\documentclass[twoside,11pt,dvipsnames, openleft]{book}

\usepackage{varioref}
\usepackage{jmlr2e}
\usepackage{bm}
\usepackage{dsfont}
\usepackage[mathscr]{euscript}
\usepackage{stmaryrd}

\renewcommand{\arraystretch}{1.15}
\usepackage[english]{babel}

\usepackage{imakeidx}
\makeindex[columns=2, title=Alphabetical Index]

\usepackage{stackengine}

\usepackage{cancel}
\usepackage{dsfont}

\usepackage{extarrows} 

\usepackage[framemethod=tikz]{mdframed}

\usepackage{amsmath}
\usepackage{arydshln}
\numberwithin{equation}{chapter}

\usepackage{blkarray}

\usepackage[absolute,overlay]{textpos}

\usepackage[final]{pdfpages}

\usepackage[noframe]{showframe}
\usepackage{framed}
\usepackage{lipsum}

\usepackage{xcolor}

\usepackage{tcolorbox}
\usepackage{graphicx}    %插入图片的宏包
\usepackage[hang]{subfigure}

\usepackage{algorithm}
\usepackage{algpseudocode}
\usepackage{graphics}
\usepackage{epsfig}
\usepackage{blkarray}
\usepackage{listings}
\usepackage{booktabs}  % for '\midrule' macro

\usepackage{tikz}
\usetikzlibrary{mindmap,trees,backgrounds}
\usetikzlibrary{arrows.meta}
\usetikzlibrary{patterns}
\usetikzlibrary{decorations.text}
\usetikzlibrary{calc,backgrounds}
\newcommand*\circled[1]{\tikz[baseline=(char.base)]{\node[shape=circle,draw,inner sep=2pt] (char) {#1};}}
\usepackage{changepage}                 % adjust margins for selected portions
\usepackage{hhline}

\usetikzlibrary{positioning,shapes,shadows,arrows}
\tikzstyle{condition}=[rectangle, draw=black, rounded corners, fill=colorqr, drop shadow,
text centered, anchor=north, text=black, text width=3cm]
\tikzstyle{abstract}=[rectangle, draw=black, rounded corners, fill=blue!30, drop shadow,
text centered, anchor=north, text=black, text width=3cm]
\tikzstyle{comment}=[rectangle, draw=black, rounded corners, fill=color1, drop shadow,
text centered, anchor=north, text=black, text width=3cm]
\tikzstyle{myarrow}=[->, >=open triangle 90, thick]
\tikzstyle{line}=[-, thick]

\usepackage{sidecap}
\sidecaptionvpos{figure}{t}
\usepackage{verbatimbox}

\newlength{\offsetpage}
	{\end{adjustwidth}}
	{\end{adjustwidth}}

\usepackage{minitoc}   % option: un/dotted
\let\cleardoublepage\clearpage
\usepackage[titletoc]{appendix}

\definecolor{titlepagecolor}{cmyk}{75,68,67,90}
\definecolor{titlepagecolor2}{rgb}{1.0, 0.08, 0.58}
\definecolor{emerald}{rgb}{0.31, 0.78, 0.47}
\definecolor{deeppink}{HTML}{D14064}
\definecolor{lowpink}{HTML}{ffe6ec}
\newcommand{\partcolor}{gray!65} 
\definecolor{lowblue}{HTML}{E1EBFE}

\usepackage{afterpage}

\makeatletter \renewcommand*\cleardoublepage{
	\clearpage
	\if@twoside   
	\ifodd\c@page 
	\hbox{}\newpage
	\if@twocolumn\hbox{}   
	\newpage
	\fi
	\fi
	\fi
} \makeatother
\let\originalpart=\part
\def\part#1{\cleardoublepage\clearpage \pagecolor{\partcolor} \originalpart{#1}\nopagecolor }

\usepackage{blkarray}
\newenvironment{sbmatrix}[1]{\def\mysubscript{#1}\mathop\bgroup\begin{bmatrix}}{\end{bmatrix}\egroup_{\textstyle\mathstrut\mysubscript}}
\renewcommand\thesection{\thechapter.\arabic{section}}

\newenvironment{bmatrixfoot}
{\footnotesize\begin{bmatrix}}
	{\end{bmatrix}\normalsize}

\newcommand{\new}{\text{new}}
\newcommand{\MAP}{\text{MAP}}
\newcommand{\ML}{\text{ML}}
\newcommand{\AIC}{\text{AIC}}
\newcommand{\BIC}{\text{BIC}}
\newcommand{\CI}{\text{CI}}

\newcommand{\cond}{\text{cond} }

\newcommand{\holders}{{H\"older's }}

\newcommand{\topone}{{(1)}}

\newcommand{\toptminus}{{(t-1)}}

\newcommand{\toptzero}{{(t)}}
\newcommand{\toptzeroTOP}{{(t)\top}}

\newcommand{\toptone}{{(t+1)}}

\newcommand{\dom}{\mathrm{dom}}

\newcommand{\interior}{\text{int}}

\newcommand\comple[1]{#1^c}

\newcommand{\LR}{\mathrm{LR}}
\newcommand{\LLR}{\mathrm{LLR}}
\newcommand{\Deviance}{\mathrm{D}}

\newcommand{\hadaprod}{\circ}

{\endMakeFramed}

\makeatletter
\renewcommand{\l@section}{\@dottedtocline{1}{1.5em}{2.2em}}
\renewcommand{\l@subsection}{\@dottedtocline{2}{4.0em}{3.2em}}
\renewcommand{\l@subsubsection}{\@dottedtocline{3}{7.1em}{4.3em}}
\makeatother

\usepackage[labelfont=bf]{caption}
\newcommand\myhrulefill[1]{\leavevmode\leaders\hrule height#1\hfill\kern0pt}
\DeclareCaptionFormat{myformat}{
	{\color[RGB]{0,0,0}\myhrulefill{0.12em}}
	\\#1#2#3
}
\def\algoalign#1{\parbox[t]{\dimexpr\linewidth-\algorithmicindent}{#1}}

\usepackage{yfonts,color,lettrine}
\definecolor{caligraphcolor}{HTML}{74AECB}

\usepackage{setspace}
\AtBeginDocument{\setlength{\DefaultFindent}{0.5em}}
\usepackage[shortlabels]{enumitem}  % enum style [](i)]
\usepackage{enumitem}

\newcommand*{\eitemi}{\tikz \draw [baseline, ball color=structurecolor,draw=none] circle (2pt);}
\newcommand*{\eitemii}{\tikz \draw [baseline, fill=structurecolor,draw=none,circular drop shadow] circle (2pt);}
\newcommand*{\eitemiii}{\tikz \draw [baseline, fill=structurecolor,draw=none] circle (2pt);}
\setlist[enumerate,1]{label=\color{black}\arabic*.,itemsep=0pt,partopsep=0pt,parsep=\parskip,topsep=5pt}
\setlist[enumerate,2]{label=\color{black}(\alph*).,itemsep=0pt,partopsep=0pt,parsep=\parskip,topsep=5pt}
\setlist[enumerate,3]{label=\color{black}\Roman*.,itemsep=0pt,partopsep=0pt,parsep=\parskip,topsep=5pt}
\setlist[enumerate,4]{label=\color{black}\Alph*.,itemsep=0pt,partopsep=0pt,parsep=\parskip,topsep=5pt}
\setlist[itemize,1]{label={\eitemi},itemsep=0pt,partopsep=0pt,parsep=\parskip,topsep=5pt}
\setlist[itemize,2]{label={\eitemii},itemsep=0pt,partopsep=0pt,parsep=\parskip,topsep=5pt}
\setlist[itemize,3]{label={\eitemiii},itemsep=0pt,partopsep=0pt,parsep=\parskip,topsep=5pt}

\usepackage{xcolor}
\usepackage[Conny]{fncychap}
\makeatletter  

\def\mghrulefill#1{\color{black}\leavevmode\leaders\hrule\@height #1\hfill\kern\z@}
\makeatother

\definecolor{chaptertitle}{RGB}{0,0,128}
\definecolor{chapternum}{RGB}{255,0,0}
\usepackage[colorlinks]{hyperref}
\usepackage{hyperref}
\hypersetup{
	colorlinks=true, %set true if you want colored links
	linktoc=all,     %set to all if you want both sections and subsections linked
	linkcolor=mydarkblue,  %choose some color if you want links to stand out
	anchorcolor=blue,
	citecolor=mydarkgreen,
}
\usepackage[hyperpageref]{backref} %backref
\usepackage{fancyhdr, blindtext}

\newcommand{\changefonts}{%
	\fontsize{9}{11}\selectfont
}

\fancypagestyle{plain}{%

	\fancyhf{}  % Clear header and footer
}

\usepackage{amsthm}
\newtheoremstyle{normalfontstyle} % Name
{3pt}                           % Space above
{3pt}                           % Space below
{\normalfont}                   % Body font
{}                              % Indent amount
{\bfseries}                     % Theorem head font
{}                             % Punctuation after theorem head
{ }                             % Space after theorem head
{}                              % Theorem head spec (can be left empty, meaning 'normal')

\usepackage{thmtools}
\declaretheoremstyle[
spaceabove=3pt,
spacebelow=3pt,
headfont=\bfseries,
notefont=\bfseries, % This ensures the optional note is in bold
notebraces={(}{)}, % This ensures the optional note is in brackets
bodyfont=\normalfont,
postheadspace=1em, % adds a space after the theorem head.
]{normalfontboldhead}

\theoremstyle{normalfontstyle}
\newcommand{\BlackBox}{\rule{1.5ex}{1.5ex}}  % end of proof
\renewenvironment{proof}{\par\noindent{\bf Proof\ }}{\hfill\BlackBox\\[2mm]}

\declaretheorem[style=normalfontboldhead, name=Definition, numberlike=theo]{definitionT}
\newmdenv[skipabove=7pt,
skipbelow=7pt,
rightline=false,
leftline=true,
topline=false,
bottomline=false,
linecolor=mydarkblue,
innerleftmargin=5pt,
innerrightmargin=5pt,
innertopmargin=0pt,
leftmargin=2cm,
rightmargin=0cm,
linewidth=4pt,
innerbottommargin=0pt]{dBox}
\newenvironment{definition}{\begin{dBox}\begin{definitionT}}{\end{definitionT}\end{dBox}}

\declaretheorem[style=normalfontboldhead, name=Exercise, numberlike=theo]{exerciseC}
\newmdenv[skipabove=7pt,
skipbelow=7pt,
rightline=false,
leftline=true,
topline=false,
bottomline=false,
linecolor=mydarkgreen,
innerleftmargin=5pt,
innerrightmargin=5pt,
innertopmargin=0pt,
leftmargin=2cm,
rightmargin=0cm,
linewidth=4pt,
innerbottommargin=0pt]{eBox}
\newenvironment{exercise}{\begin{eBox}\begin{exerciseC}}{\end{exerciseC}\end{eBox}}

\declaretheorem[style=normalfontboldhead, name=Remark, numberlike=theo]{remarekC}
\newmdenv[skipabove=7pt,
skipbelow=7pt,
rightline=false,
leftline=true,
topline=false,
bottomline=false,
linecolor=mydarkpurple,
innerleftmargin=5pt,
innerrightmargin=5pt,
innertopmargin=0pt,
leftmargin=2cm,
rightmargin=0cm,
linewidth=4pt,
innerbottommargin=0pt]{rBox}
\newenvironment{remark}{\begin{rBox}\begin{remarekC}}{\end{remarekC}\end{rBox}}

\declaretheorem[style=normalfontboldhead, name=Assumption, numberlike=theo]{assumptionC}
\newmdenv[skipabove=7pt,
skipbelow=7pt,
rightline=false,
leftline=true,
topline=false,
bottomline=false,
linecolor=mydarkpurple,
innerleftmargin=5pt,
innerrightmargin=5pt,
innertopmargin=0pt,
leftmargin=2cm,
rightmargin=0cm,
linewidth=4pt,
innerbottommargin=0pt]{asBox}
\newenvironment{assumption}{\begin{asBox}\begin{assumptionC}}{\end{assumptionC}\end{asBox}}

\declaretheorem[style=normalfontboldhead, name=Example, numberlike=theo]{exampleC}
\newmdenv[skipabove=7pt,
skipbelow=7pt,
rightline=false,
leftline=false,
topline=false,
bottomline=false,
linecolor=mydarkgreen,
innerleftmargin=1pt,
innerrightmargin=5pt,
innertopmargin=0pt,
leftmargin=2cm,
rightmargin=0cm,
linewidth=4pt,
innerbottommargin=0pt]{xBox}
\newenvironment{example}{\begin{xBox}\begin{exampleC}}{\exampbar\end{exampleC}\end{xBox}}

\usepackage{adforn}  % symbols of \adftripleflourishleft
\newcommand{\xchaptertitle}{Chapter~\thechapter~}
\newcommand{\problemname}{Problems}
\newenvironment{problemset}[1][\xchaptertitle~\problemname]{
	\vspace*{10pt}
	\begin{center}
		\phantomsection\addcontentsline{toc}{section}{\texorpdfstring{\xchaptertitle~\problemname}{\problemname}}
		\markright{#1}
		\textcolor{structurecolor}{\Large\bfseries\adftripleflourishleft~#1~\adftripleflourishright}
	\end{center}
	\begin{enumerate}[ref=\thechapter.\theenumi]}{
\end{enumerate}}

\usepackage{color}   %May be necessary if you want to color links

\definecolor{winestain}{rgb}{0.5,0,0}
\definecolor{colorGreenOcre}{RGB}{51,102,0} 
\definecolor{colorBlue2}{RGB}{200,207,248}
\definecolor{mydarkblue}{rgb}{0,0.08,0.45}

\newcommand{\mdframecolor}{gray!10}
\newcommand{\mdframehideline}{true}
\definecolor{mylightbluetitle}{RGB}{60,113,183}
\definecolor{mylightbluetext}{rgb}{0,0.08,0.45}
\definecolor{structurecolorblue}{RGB}{60,113,183}
\definecolor{structurecolorgreen}{RGB}{63,145,182}

\colorlet{structurecolor}{structurecolorblue}
\definecolor{structurecolorelegant}{RGB}{60,113,183}
\definecolor{structurecolorlt}{RGB}{31,119,185}

\definecolor{structurecolorHighTheoremBlue}{RGB}{220,227,248}
\definecolor{structurecolorHighTheoremGreen}{RGB}{188,222,231}
\colorlet{structurecolorHighTheorem}{structurecolorHighTheoremBlue}
\definecolor{mdframecolorRemark}{RGB}{186,94,103}
\definecolor{mydarkblue}{rgb}{0,0.08,0.45}
\definecolor{mydarkred}{rgb}{0.70,0.00,0.00}
\definecolor{mydarkgreen}{rgb}{0.00,0.30,0.00}
\definecolor{mydarkyellow}{RGB}{197,151,13}
\definecolor{mydarkpurple}{RGB}{90,35,140}
\definecolor{mydarkgray}{RGB}{64,64,64}

\definecolor{color0}  {RGB}{174,225,254} %%% text color
\definecolor{color1}  {RGB}{220,227,248} %%% title and subsection color
\definecolor{color2}  {RGB}{28,130,185} %%% section color
\definecolor{color3}  {RGB}{255,253,250} %%% background color
\definecolor{colormiddleright}  {RGB}{245,253,250} 
\definecolor{colorbottomleft}  {RGB}{255,243,250} 
\definecolor{coloruppermiddle}  {RGB}{255,253,230} 
\definecolor{colormiddleleft}  {RGB}{255,244,237}
\definecolor{colorcr}  {RGB}{249,253,232} 
\definecolor{colorreduction}  {RGB}{255,235,254} 
\definecolor{colorqr}  {RGB}{254,221,199} 
\definecolor{colorbiconjugate}  {RGB}{251,149,161} 
\definecolor{colorsvd}  {RGB}{215,247,235} 
\definecolor{colorupperright}  {RGB}{239,246,251} 
\definecolor{colorspectral}  {RGB}{206,226,243} 
\definecolor{colorbottomright}  {RGB}{220,224,236} 
\definecolor{coloreigenvalue}  {RGB}{197,203,224} 
\definecolor{colorcp} {RGB}{217, 234, 186} 
\definecolor{colorcpborder} {RGB}{233, 243, 216} 
\definecolor{colorupperleft}  {RGB}{235,243,240} 
\definecolor{colorsemidefinite}  {RGB}{217,232,226} 
\definecolor{colormiddle} {RGB}{235, 240,255}
\definecolor{colorlu}  {RGB}{220,227,255} 
\definecolor{colorals}  {RGB}{240,230,255} 
\definecolor{coloralsbkg}  {RGB}{248,243,255} 
\definecolor{canaryyellow}{rgb}{1.0, 0.75, 0.0}
\definecolor{bluepigment}{rgb}{0.0, 0.0, 1.0}
\definecolor{canarypurple}{RGB}{208, 13, 241}
\definecolor{colorGreenOcre}{RGB}{51,102,0} 
\definecolor{colorBlue1}  {RGB}{220,227,248}
\definecolor{colorBlue2}{RGB}{200,207,248}
\definecolor{shadecolor}{gray}{0.75}

\definecolor{color0}  {RGB}{174,225,254} %%% text color
\definecolor{color1}  {RGB}{220,227,248} %%% title and subsection color
\definecolor{color2}  {RGB}{28,130,185} %%% section color
\definecolor{color3}  {RGB}{255,253,250} %%% background color
\definecolor{color0}  {RGB}{174,225,254} %%% text color
\definecolor{color1}  {RGB}{220,227,248} %%% title and subsection color
\definecolor{color2}  {RGB}{28,130,185} %%% section color
\definecolor{color3}  {RGB}{255,253,250} %%% background color

\definecolor{colormiddleright}  {RGB}{245,253,250} 
\definecolor{colorbottomleft}  {RGB}{255,243,250} 
\definecolor{coloruppermiddle}  {RGB}{255,253,230} 
\definecolor{colormiddleleft}  {RGB}{255,244,237}

\definecolor{colorcr}  {RGB}{249,253,232} 
\definecolor{colorreduction}  {RGB}{255,235,254} 
\definecolor{colorqr}  {RGB}{254,221,199} 
\definecolor{colorbiconjugate}  {RGB}{251,149,161} 

\definecolor{colorsvd}  {RGB}{215,247,235} 

\definecolor{colorupperright}  {RGB}{239,246,251} 
\definecolor{colorspectral}  {RGB}{206,226,243} 

\definecolor{colorbottomright}  {RGB}{220,224,236} 
\definecolor{coloreigenvalue}  {RGB}{197,203,224} 

\definecolor{colorcp} {RGB}{217, 234, 186} 
\definecolor{colorcpborder} {RGB}{233, 243, 216} 

\definecolor{colorupperleft}  {RGB}{235,243,240} 
\definecolor{colorsemidefinite}  {RGB}{217,232,226} 

\definecolor{colormiddle} {RGB}{235, 240,255}
\definecolor{colorlu}  {RGB}{220,227,255} 

\definecolor{colorals}  {RGB}{240,230,255} 
\definecolor{coloralsbkg}  {RGB}{248,243,255} 

\definecolor{canaryyellow}{rgb}{1.0, 0.75, 0.0}
\definecolor{bluepigment}{rgb}{0.0, 0.0, 1.0}

\definecolor{canarypurple}{RGB}{208, 13, 241}
\definecolor{brightlavender}{rgb}{0.44, 0.16, 0.39}
\newcommand{\uniformdist}{\text{Uniform}}
\newcommand{\exponential}{\mathcal{E}}

\newcommand{\gammadist}{\mathcal{G}}
\newcommand{\inversegammadist}{\mathcal{G}^{-1}}

\newcommand{\studentt}{\tau}
\newcommand{\normal}{\mathcal{N}}
\newcommand{\laplacedist}{\mathcal{L}}

\newcommand{\betadist}{\mathrm{Beta}}
\newcommand{\wishartdist}{\mathrm{Wi}}
\newcommand{\inversewishart}{\mathrm{IW}}
\newcommand{\inversechidist}{\mathrm{\chi^{-2}}}
\newcommand{\normalinversegamma}{\mathcal{NIG}}

\newcommand{\nig}{\mathcal{NIG}}
\newcommand{\chisquared}{\chi^2}
\newcommand{\poissondist}{\mathcal{P}}

\newcommand{\real}{\mathbb{R}}
\newcommand{\prob}{\Pr}
\newcommand{\leadto}{\qquad\underrightarrow{ \text{leads to} }\qquad}
\newcommand{\leadtosmall}{\,\,\,\,\underrightarrow{ \text{leads to} }\,\,\,\,}

\mathchardef\mhyphen="2D
\newcommand{\integer}{\mathbb{Z}}
\newcommand{\gap}{\,\,\,\,\,\,\,\,}

\newcommand{\diag}{\mathrm{diag}}

\newcommand{\indicator}{\mathds{1}}

\newcommand{\bgast}{\boldsymbol{g}_\ast}

\newcommand{\multinomial}{\mathrm{Multi}}
\newcommand{\bernoulli}{\mathrm{Bern}}      % Bernoulli distribution
\newcommand{\bernoullidist}{\mathrm{Bern}}

\newcommand{\smu}{\mu}
\newcommand{\ssigma}{\sigma}

\newcommand{\tr}{\mathrm{tr}}

\newcommand{\binoimial}{\mathrm{Binom}}
\newcommand{\binomialdist}{\mathrm{Binom}}
\usepackage{amsmath}% http://ctan.org/pkg/amsmath

\newcommand{\trans}[1]{\ensuremath{#1^{ \top}}}

\newcommand{\exampbar}{\hfill $\square$\par}
\newcommand{\Stackd}{\stackrel{d}{\longrightarrow}}

\newcommand{\Stackp}{\stackrel{p}{\longrightarrow}}

\newcommand{\bXast}{\boldsymbol{X}_\ast}

\newcommand{\cspace}{\mathcal{C}}
\newcommand{\nspace}{\mathcal{N}}

\newcommand{\Corr}{\mathbb{C}\mathrm{orr}}
\newcommand{\Cov}{\mathbb{C}\mathrm{ov}}
\newcommand{\cov}{\mathrm{Cov}}
\newcommand{\Exp}{\mathbb{E}}
\newcommand{\Var}{\mathbb{V}\mathrm{ar}}
\newcommand{\fisher}{\mathbb{I}}
\newcommand{\argmax}{\operatorname*{\text{arg max}}}
\newcommand{\argmin}{\operatorname*{\text{arg min}}} 

\newcommand\abs[1]{\left\lvert#1\right\rvert}
\newcommand\absbig[1]{\big\lvert#1\big\rvert}
\newcommand\norm[1]{\left\lVert#1\right\rVert}

\newcommand\normone[1]{\left\lVert#1\right\rVert_1}
\newcommand\normonebig[1]{\big\lVert#1\big\rVert_1}

\newcommand\normtwo[1]{\left\lVert#1\right\rVert_2}
\newcommand\normtwobig[1]{\big\lVert#1\big\rVert_2}
\newcommand\normp[1]{\left\lVert#1\right\rVert_p}

\newcommand\normf[1]{\left\lVert#1\right\rVert_F}
\newcommand\normfbig[1]{\big\lVert#1\big\rVert_F}
\newcommand\norminf[1]{\left\lVert#1\right\rVert_{\infty}}
\newcommand\norminfbig[1]{\big\lVert#1\big\rVert_{\infty}}

\newcommand\innerproduct[1]{\left\langle#1\right\rangle}
\newcommand{\ESS}{\mathrm{ESS}}
\newcommand{\RSS}{\mathrm{RSS}}
\newcommand{\TSS}{\mathrm{TSS}}
\newcommand{\bias}{\mathrm{Bias}}
\newcommand{\MSE}{\mathrm{MSE}}
\newcommand{\sgn}{\mathrm{sgn}}
\newcommand{\rank}{\mathrm{rank}}
\newcommand{\nnz}{\mathrm{nnz}}
\newcommand{\trace}{\mathrm{tr}}

\newcommand{\spn}{\mathrm{span}}

\newcommand{\gapthree}{\,\,\,}

\newcommand{\gapforall}{\,\,}
\mathchardef\mhyphen="2D
\newcommand{\complex}{\mathbb{C}}
\newcommand{\naturalset}{\mathbb{N}}

\def\1{\bm{1}}

\newcommand{\R}{\mathbb{R}}

\newcommand{\entropy}{\mathrm{H}}
\newcommand{\KL}{D_{\mathrm{KL}}}
\DeclareMathOperator{\sign}{sign}

\newcommand{\bzero}{\boldsymbol{0}}
\newcommand{\balpha}{\boldsymbol\alpha}
\newcommand{\bbeta}{\boldsymbol\beta}
\newcommand{\bdelta}{\boldsymbol\delta}
\newcommand{\bgamma}{\boldsymbol\gamma}
\newcommand{\bepsilon}{\boldsymbol\epsilon}
\newcommand{\boldeta}{\boldsymbol\eta}
\newcommand{\bphiast}{\bphi_\ast}
\newcommand{\bmu}{\boldsymbol\mu}
\newcommand{\bxi}{\boldsymbol\xi}

\newcommand{\bsigma}{\boldsymbol\sigma}
\newcommand{\bSigma}{\boldsymbol\Sigma}

\newcommand{\bOmega}{\boldsymbol\Omega}
\newcommand{\blambda}{\boldsymbol\lambda}
\newcommand{\bLambda}{\boldsymbol\Lambda}
\newcommand{\btheta}{\boldsymbol\theta}

\newcommand{\bphi}{\boldsymbol\phi}
\newcommand{\bPhi}{\boldsymbol\Phi}

\newcommand{\widehattheta}{\widehat{\theta}}
\newcommand{\widehatbtheta}{\widehat{\btheta}}
\newcommand{\widetildebtheta}{\widetilde{\btheta}}
\newcommand{\widehatbbeta}{\widehat{\bbeta}}
\newcommand{\widetildebbeta}{\widetilde{\bbeta}}
\newcommand{\widehatbalpha}{\widehat{\balpha}}
\newcommand{\widetildebalpha}{\widetilde{\balpha}}
\newcommand{\widehatbeta}{\widehat{\beta}}

\newcommand{\widehaty}{\widehat{y}}

\newcommand{\widehatbB}{\widehat{\bm{B}}}

\newcommand{\widehatbV}{\widehat{\bm{V}}}

\newcommand{\widehatbX}{\widehat{\bm{X}}}

\newcommand{\widehatba}{\widehat{\bm{a}}}
\newcommand{\widehatbb}{\widehat{\bm{b}}}

\newcommand{\widehatby}{\widehat{\bm{y}}}

\newcommand{\mathcalC}{\mathcal{C}}
\newcommand{\mathcalD}{\mathcal{D}}

\newcommand{\mathcalF}{\mathcal{F}}

\newcommand{\mathcalH}{\mathcal{H}}

\newcommand{\mathcalL}{\mathcal{L}}

\newcommand{\mathcalN}{\mathcal{N}}
\newcommand{\mathcalO}{\mathcal{O}}

\newcommand{\mathcalS}{\mathcal{S}}

\newcommand{\mathcalV}{\mathcal{V}}
\newcommand{\mathcalW}{\mathcal{W}}
\newcommand{\mathcalX}{\mathcal{X}}
\newcommand{\mathcalY}{\mathcal{Y}}

\newcommand{\widetildebA}{\widetilde{\bm{A}}}
\newcommand{\widetildebB}{\widetilde{\bm{B}}}
\newcommand{\widetildebC}{\widetilde{\bm{C}}}
\newcommand{\widetildebD}{\widetilde{\bm{D}}}

\newcommand{\widetildebP}{\widetilde{\bm{P}}}
\newcommand{\widetildebQ}{\widetilde{\bm{Q}}}
\newcommand{\widetildebR}{\widetilde{\bm{R}}}
\newcommand{\widetildebS}{\widetilde{\bm{S}}}

\newcommand{\widetildebW}{\widetilde{\bm{W}}}
\newcommand{\widetildebX}{\widetilde{\bm{X}}}

\newcommand{\widetildebZ}{\widetilde{\bm{Z}}}
\newcommand{\widetildeba}{\widetilde{\bm{a}}}

\newcommand{\widetildebd}{\widetilde{\bm{d}}}

\newcommand{\widetildebq}{\widetilde{\bm{q}}}

\newcommand{\widetildebu}{\widetilde{\bm{u}}}

\newcommand{\widetildebx}{\widetilde{\bm{x}}}
\newcommand{\widetildeby}{\widetilde{\bm{y}}}
\newcommand{\widetildebz}{\widetilde{\bm{z}}}

\newcommand{\widetildef}{\widetilde{f}}

\newcommand{\widetildep}{\widetilde{p}}

\newcommand{\bone}{\bm{1}}
\newcommand{\ba}{\bm{a}}
\newcommand{\bA}{\bm{A}}
\newcommand{\bb}{\bm{b}}
\newcommand{\bB}{\bm{B}}
\newcommand{\bc}{\bm{c}}
\newcommand{\bC}{\bm{C}}
\newcommand{\bd}{\bm{d}}
\newcommand{\bD}{\bm{D}}
\newcommand{\be}{\bm{e}}
\newcommand{\bE}{\bm{E}}

\newcommand{\bF}{\bm{F}}
\newcommand{\bg}{\bm{g}}
\newcommand{\bG}{\bm{G}}
\newcommand{\bh}{\bm{h}}
\newcommand{\bH}{\bm{H}}

\newcommand{\bI}{\bm{I}}

\newcommand{\bJ}{\bm{J}}

\newcommand{\bK}{\bm{K}}
\newcommand{\bl}{\bm{l}}
\newcommand{\bL}{\bm{L}}
\newcommand{\bmm}{\bm{m}}
\newcommand{\bM}{\bm{M}}
\newcommand{\bn}{\bm{n}}

\newcommand{\bo}{\bm{o}}

\newcommand{\bp}{\bm{p}}
\newcommand{\bP}{\bm{P}}
\newcommand{\bq}{\bm{q}}
\newcommand{\bQ}{\bm{Q}}
\newcommand{\br}{\bm{r}}
\newcommand{\bR}{\bm{R}}
\newcommand{\bs}{\bm{s}}
\newcommand{\bS}{\bm{S}}
\newcommand{\bt}{\bm{t}}
\newcommand{\bT}{\bm{T}}
\newcommand{\bu}{\bm{u}}
\newcommand{\bU}{\bm{U}}
\newcommand{\bv}{\bm{v}}
\newcommand{\bV}{\bm{V}}
\newcommand{\bw}{\bm{w}}
\newcommand{\bW}{\bm{W}}
\newcommand{\bx}{\bm{x}}
\newcommand{\bX}{\bm{X}}
\newcommand{\by}{\bm{y}}
\newcommand{\bY}{\bm{Y}}
\newcommand{\bz}{\bm{z}}
\newcommand{\bZ}{\bm{Z}}

\def\vmu{{\bm{\mu}}}
\def\vtheta{{\bm{\theta}}}
\def\va{{\bm{a}}}

\def\ve{{\bm{e}}}

\def\vx{{\bm{x}}}

\def\mA{{\bm{A}}}
\def\mB{{\bm{B}}}

\def\mH{{\bm{H}}}

\def\mJ{{\bm{J}}}

\def\mX{{\bm{X}}}

\def\mSigma{{\bm{\Sigma}}}

\DeclareMathAlphabet{\mathsfit}{\encodingdefault}{\sfdefault}{m}{sl}
\SetMathAlphabet{\mathsfit}{bold}{\encodingdefault}{\sfdefault}{bx}{n}

\def\sA{{\mathbb{A}}}
\def\sB{{\mathbb{B}}}

\def\sF{{\mathbb{F}}}

\def\sI{{\mathbb{I}}}
\def\sJ{{\mathbb{J}}}
\def\sK{{\mathbb{K}}}

\def\sN{{\mathbb{N}}}

\def\sS{{\mathbb{S}}}

\def\ra{{\textnormal{a}}}
\def\rb{{\textnormal{b}}}
\def\rc{{\textnormal{c}}}

\def\re{{\textnormal{e}}}

\def\rg{{\textnormal{g}}}

\def\rs{{\textnormal{s}}}

\def\rx{{\textnormal{x}}}
\def\ry{{\textnormal{y}}}
\def\rz{{\textnormal{z}}}

\def\rS{{\textnormal{S}}}
\def\rW{{\textnormal{W}}}
\def\rT{{\textnormal{T}}}

\def\rva{{\mathbf{a}}}
\def\rvb{{\mathbf{b}}}

\def\rve{{\mathbf{e}}}

\def\rvs{{\mathbf{s}}}

\def\rvv{{\mathbf{v}}}
\def\rvw{{\mathbf{w}}}
\def\rvx{{\mathbf{x}}}
\def\rvy{{\mathbf{y}}}
\def\rvz{{\mathbf{z}}}

\def\rmA{{\mathbf{A}}}
\def\rmB{{\mathbf{B}}}

\def\rmG{{\mathbf{G}}}

\def\rmM{{\mathbf{M}}}

\def\rmS{{\mathbf{S}}}

\def\rmX{{\mathbf{X}}}
\def\rmY{{\mathbf{Y}}}
\def\rmZ{{\mathbf{Z}}}

\def\eva{{a}}

\firstpageno{1}

\newcommand{\mytitle}{A Rigorous Introduction to Linear Models}
\begin{document}
%\includepdf[pages=-,scale=1.1,offset=0mm 0mm,]{covers/coverfrontcompressed.pdf} 
\newpage
\thispagestyle{empty}  % do not heading for first pages 
\title{\mytitle}

\author{
\begin{center}
\name Jun Lu \\ 
\email jun.lu.locky@gmail.com
\end{center}
}

\frontmatter

\newpage 
\maketitle

\chapter*{\centering \begin{normalsize}Preface\end{normalsize}}

This book aims to provide an introduction to linear models and the theoretical foundations that underpin them. Our goal is to offer a rigorous treatment of the subject for readers who already have some familiarity with ordinary least squares (OLS) regression.

In machine learning, where outputs often involve nonlinear functions, and deep learning seeks to capture complex nonlinear relationships through multiple computational layers, the underlying principles still rest on simple linear models. This highlights the importance of understanding the theory and application of linear models as a basis for more advanced techniques.
The book then covers various aspects of linear models, with a particular emphasis on the method of least squares as the primary tool for solving regression problems. Least squares minimizes the sum of squared errors to estimate the regression function that yields the smallest expected squared error.

Primarily serving as a concise yet comprehensive overview, this book emphasizes the significance of key theoretical concepts behind linear models, including distribution theory, minimum variance estimation, and analysis of variance.
We begin with the ordinary least squares framework and explore it from multiple perspectives. We then introduce random disturbances modeled as Gaussian noise. This leads naturally to the concept of likelihood, enabling us to define the maximum likelihood estimator and develop corresponding distributional theories.
The distribution theory of least squares allows us to address a variety of statistical questions and introduces several practical applications. We also prove that the least squares estimator is the best unbiased linear estimator in terms of mean squared error---moreover, it approaches the theoretical performance limit.
Finally, we extend our discussion to include Bayesian approaches to linear models and touch upon related advanced topics.

The mathematical prerequisites for this book are modest: a first course in linear algebra and basic statistics. Beyond this, the development is self-contained, with detailed and rigorous proofs provided throughout.

%Linear models have become fundamental in machine learning, especially when combined into more complex structures such as neural networks, which concatenate many simple linear models to represent highly nonlinear functions.

The main objective of this book is to offer a self-contained introduction to the core concepts, mathematical tools, and rigorous analysis behind linear models, enabling a smooth transition to the discussion of their methods and applications in subsequent sections.
However, we acknowledge the limitations of this work---we cannot cover all useful or interesting results related to linear models. Due to space constraints, we do not include separate treatments of certain topics such as LASSO and ridge regression. For a more in-depth exploration of these subjects, we refer readers to specialized literature on linear models and regularization techniques.

\newpage

\chapter*{\centering \begin{normalsize}Keywords\end{normalsize}}
Fundamental theory of linear algebra, Orthogonal projection matrix, Out-of-sample error, Gauss-Markov,  Cram\'er-Rao lower bound (CRLB), Minimum variance unbiased estimator, Distribution theory of linear models, Asymptotic theory, Connection between Bayesian approach and Gaussian process, Variable selection, Large-scale optimization, Generalized linear models.

%\tableofcontents
\newpage
\begingroup
\hypersetup{
linkcolor=structurecolor,
linktoc=page,  % page: only the page will be colored; section, all, none etc
}
\dominitoc
\pdfbookmark{\contentsname}{toc} %% att TOC to bookmark
\tableofcontents 

\endgroup

\chapter*{Notation}\label{notation}
\index{Notation}

%\typeout{START_CHAPTER "notation" \theabspage}

% Sometimes we have to include the following line to get this section
% included in the Table of Contents despite being a chapter*
%\addcontentsline{toc}{chapter}{Notation}
This section provides a concise reference describing notation used throughout this
book.
If you are unfamiliar with any of the corresponding mathematical concepts,
the book describes most of these ideas in Chapter~\ref{chapter_lsintroduction} (p.~\pageref{chapter_lsintroduction}).

\vspace{0.4in}
%\vspace{\notationgap}
% Need to use minipage to keep title of table on same page as table
\begin{minipage}{\textwidth}
% This is a hack to put a little title over the table
% We cannot use "\section*", etc., they appear in the table of contents.
% tocdepth does not work on this chapter.
\centerline{\bf Numbers and Arrays}
\bgroup
% The \arraystretch definition here increases the space between rows in the table,
% so that \displaystyle math has more vertical space.
\def\arraystretch{1.5}
\begin{tabular}{cp{4.25in}}
$\displaystyle a$   & A scalar (integer or real)\\
$\displaystyle \ba$ & A vector\\
$\displaystyle \bA$ & A matrix\\
%$\displaystyle \eA$ & A tensor\\
$\displaystyle \bI_n$ & Identity matrix with $n$ rows and $n$ columns\\
$\displaystyle \bI$   & Identity matrix with dimensionality implied by context\\
$\displaystyle \ve_i$ & Standard basis vector $[0,\dots,0,1,0,\dots,0]$ with a 1 at position $i$\\
$\displaystyle \text{diag}(\va)$ & A square, diagonal matrix with diagonal entries given by $\va$\\
$\displaystyle \ra$   & A scalar random variable\\
$\displaystyle \rva$  & A vector-valued random variable\\
$\displaystyle \rmA$  & A matrix-valued random variable\\
\end{tabular}
\egroup
\index{Scalar}
\index{Vector}
\index{Matrix}
\index{Tensor}
\end{minipage}

\index{Sets}
\vspace{0.2in}
%\vspace{\notationgap}
\begin{minipage}{\textwidth}
\centerline{\bf Sets}
\bgroup
\def\arraystretch{1.5}
\begin{tabular}{cp{4.25in}}
$\displaystyle \sA$ & A set\\
$\displaystyle \varnothing$ & The null set \\
$\displaystyle \real, \complex, \sF\equiv \{\real \text{ or }\complex\}$ & The set of real, complex, either real or complex numbers\\
$\displaystyle \naturalset$ & The set of natural numbers \\
%$\displaystyle \complex$ & The set of complex numbers \\
% NOTE: do not use \R^+, because it is ambiguous whether:
% - It includes 0
% - It includes only real numbers, or also infinity.
% We usually do not include infinity, so we may explicitly write
% [0, \infty) to include 0
% (0, \infty) to not include 0
$\displaystyle \{0, 1\}$ & The set containing 0 and 1 \\
$\displaystyle \{0, 1, \dots, n \}$ & The set of all integers between $0$ and $n$\\
$\displaystyle [a, b]$ & The real interval including $a$ and $b$\\
$\displaystyle (a, b]$ & The real interval excluding $a$ but including $b$\\
$\displaystyle \sA \backslash \sB$ & Set subtraction, i.e., the set containing the elements of $\sA$ that are not in $\sB$\\
%$\displaystyle \gG$ & A graph\\
%$\displaystyle \parents_\gG(\ervx_i)$ & The parents of $\ervx_i$ in $\gG$
\end{tabular}
\egroup
\index{Scalar}
\index{Vector}
\index{Matrix}
\index{Tensor}
\index{Graph}
\index{Set}
\end{minipage}

\index{Matrix indexing}
\vspace{0.2in}
%\vspace{\notationgap}
\begin{minipage}{\textwidth}
\centerline{\bf Indexing}
\bgroup
\def\arraystretch{1.5}
\begin{tabular}{cp{4.25in}}
$\displaystyle \eva_i$ & Element $i$ of vector $\va$, with indexing starting at 1 \\
$\displaystyle \ba_{-i}$ & All elements of vector $\va$ except for element $i$ \\
$\displaystyle  a_{ij}$ & Element $i, j$ of matrix $\mA$ \\
$\displaystyle \mA_{i, :}=\mA[i,:],\, \ba^{(i)}$ & Row $i$ of matrix $\mA$ \\
$\displaystyle \mA_{:, i}=\mA[:, i],\, \ba_i$ & Column $i$ of matrix $\mA$ \\
%$\displaystyle \eA_{ijk}, \elA_{ijk}, a_{ijk}$ & Element $(i, j, k)$ of a 3-D tensor $\eA$\\
%$\displaystyle \eA_{:, :, i}$ & 2-D slice of a 3-D tensor $\eA$\\
%$\displaystyle \erva_i$ & Element $i$ of the random vector $\rva$ \\
\end{tabular}
\egroup
\end{minipage}

\vspace{0.2in}
%\vspace{\notationgap}
\begin{minipage}{\textwidth}
\centerline{\bf Linear Algebra Operations}
\bgroup
\def\arraystretch{1.5}
\begin{tabular}{cp{4.25in}}
$\displaystyle \bA^\top$ & Transpose of matrix $\mA$ \\
$\displaystyle \bA^+$ & Moore-Penrose pseudo-inverse of $\mA$\\
$\displaystyle \bA \hadaprod \bB $ & Element-wise (Hadamard) product of $\mA$ and $\mB$ \\
% Wikipedia uses \circ for element-wise multiplication but this could be confused with function composition
$\displaystyle \mathrm{det}(\bA)$ & Determinant of $\bA$ \\
%$\displaystyle \mathrm{tr}(\bA)$ & Trace of $\bA$ \\
$\displaystyle \mathrm{rref}(\bA)$ & Reduced row echelon form of $\bA$ \\
$\displaystyle \cspace(\bA)$ & Column space of $\bA$ \\
$\displaystyle \nspace(\bA)$ & Null space of $\bA$ \\
$\displaystyle \mathcalV$ & A general subspace \\
$\displaystyle \rank(\bA)$ & Rank of $\bA$ \\
$\displaystyle \trace(\bA)$ & Trace of $\bA$ \\
\end{tabular}
\egroup
\index{Transpose}
\index{Element-wise product, Hadamard product}
\index{Hadamard product}
\index{Determinant}
\end{minipage}

\vspace{0.4in}
%\vspace{\notationgap}
\begin{minipage}{\textwidth}
\centerline{\bf Calculus}
\bgroup
\def\arraystretch{1.5}
\begin{tabular}{cp{4.25in}}
% NOTE: the [2ex] on the next line adds extra height to that row of the table.
% Without that command, the fraction on the first line is too tall and collides
% with the fraction on the second line.
$\displaystyle\frac{d y} {d x}$ & Derivative of $y$ with respect to $x$\\ [2ex]
$\displaystyle \frac{\partial y} {\partial x} $ & Partial derivative of $y$ with respect to $x$ \\
$\displaystyle \nabla_{\bx} y $ & Gradient of $y$ with respect to $\bx$ \\
$\displaystyle \nabla_{\bX} y $ & Matrix derivatives of $y$ with respect to $\bX$ \\
%$\displaystyle \nabla_{\eX} y $ & Tensor containing derivatives of $y$ with respect to $\eX$ \\
$\displaystyle \frac{\partial f}{\partial \vx} $ & Jacobian matrix $\mJ \in \R^{m\times n}$ of $f: \R^n \rightarrow \R^m$\\
$\displaystyle \nabla_\vx^2 f(\vx)\text{ or }\mH( f)(\vx)$ & The Hessian matrix of $f$ at input point $\vx$\\
$\displaystyle \int f(\vx) d\vx $ & Definite integral over the entire domain of $\vx$ \\
$\displaystyle \int_\sS f(\vx) d\vx$ & Definite integral with respect to $\vx$ over the set $\sS$ \\
\end{tabular}
\egroup
\index{Derivative}
\index{Integral}
\index{Jacobian matrix}
\index{Hessian matrix}
\end{minipage}

\vspace{0.4in}
%\vspace{\notationgap}
\begin{minipage}{\textwidth}
\centerline{\bf Probability and Information Theory}
\bgroup
\def\arraystretch{1.5}
\begin{tabular}{cp{4.25in}}
$\displaystyle \ra \bot \rb$ & The random variables $\ra$ and $\rb$ are independent\\
$\displaystyle \ra \bot \rb \mid \rc $ & They are conditionally independent given $\rc$\\
$\displaystyle \Pr(\bx)$ & A probability distribution over a discrete variable\\
$\displaystyle p(\bx), p_{\rvx}(\bx), f(\bx), f_{\rvx}(\bx)$ & A probability distribution over a continuous variable, or over
a variable whose type has not been specified\\
$\displaystyle \ra \sim P$ & Random variable $\ra$ has distribution $P$\\% so thing on left of \sim should always be a random variable, with name beginning with \r
$\displaystyle  \Exp_{\rx\sim P} [ f(x) ]\text{ or } \Exp [f(x)]$ & Expectation of $f(x)$ with respect to $P(\rx)$ \\
$\displaystyle \Var[f(x)] $ &  Variance of $f(x)$ under $P(\rx)$ \\
$\displaystyle \Cov[f(x),g(x)] $ & Covariance of $f(x)$ and $g(x)$ under $P(\rx)$\\
$\displaystyle \Corr[f(x),g(x)] $ & Correlation of $f(x)$ and $g(x)$ under $P(\rx)$\\
$\displaystyle H(\rx) $ & Shannon entropy of the random variable $\rx$\\
$\displaystyle \KL [P \parallel Q] $ & Kullback-Leibler divergence of P and Q \\
$\displaystyle \mathcal{N} ( \vx \mid \vmu , \mSigma)$ & Gaussian distribution %
over $\vx$ with mean $\vmu$ and covariance $\mSigma$ \\
$\displaystyle \bernoullidist(p) $ & Bernoulli distribution with mean $p$ \\
$\displaystyle \laplacedist(\mu,b) $ & Laplace distribution with location $\mu$ and scale $b$ \\
$\displaystyle \exponential(\lambda) $ & Exponential distribution with scale $\lambda$ \\
$\displaystyle \poissondist(\lambda) $ & Poisson distribution with rate $\lambda$ \\
$\displaystyle \chisquared_{(p)} $ & Chi-squared distribution with  $p$ degrees of freedom (df)\\
$\displaystyle \wishartdist(\bM,\nu) $ & Wishart distribution with scale  $\bM$ and  df $\nu$ \\
$\displaystyle \inversewishart(\bS,\nu) $ & Inverse-Wishart distribution with scale  $\bS$ and  df $\nu$ \\
\end{tabular}
\egroup
\index{Independence}
\index{Conditional independence}
\index{Variance}
\index{Covariance}
\index{Kullback-Leibler divergence}
\index{Shannon entropy}
\end{minipage}

\vspace{0.4in}
%\vspace{\notationgap}
\begin{minipage}{\textwidth}
\centerline{\bf Functions}
\bgroup
\def\arraystretch{1.5}
\begin{tabular}{cp{4.25in}}
$\displaystyle f: \sA \rightarrow \sB$ & The function $f$ with domain $\sA$ and range $\sB$\\
$\displaystyle f \circ g $ & Composition of the functions $f$ and $g$ \\
  $\displaystyle f(\vx ; \vtheta) $ & A function of $\vx$ parametrized by $\vtheta$.
  (Sometimes we write $f(\vx)$ and omit the argument $\vtheta$ to lighten notation) \\
$\displaystyle \ln(x),\log(x)$ & Natural logarithm of $x$ \\
$\displaystyle \sigma(x),\, \text{Sigmoid}(x)$ & Logistic sigmoid, i.e., $\displaystyle \frac{1} {1 + \exp\{-x\}}$ \\
$\displaystyle \text{logit}(\pi)$ & Logit function, i.e. $\text{logit}(\pi) = \ln(\pi/(1-\pi))$, where $\pi\in(0,1)$ \\
$\displaystyle \zeta(x)$ & Softplus, $\log(1 + \exp\{x\})$ \\
$\displaystyle \norm{\bx}_p, \norm{\bx}_s $ & $\ell_p$ norm of $\vx$ \\
$\displaystyle \norm{\bx}=\normtwo{\bx} $ & $\ell_2$ norm of $\vx$ \\
$\displaystyle \norm{\bx}=\normone{\bx} $ & $\ell_1$ norm of $\vx$ \\
$\displaystyle \norm{\bx}=\norminf{\bx} $ & $\ell_\infty$ norm of $\vx$ \\
$\displaystyle [x]_+$ & Positive part of $x$, i.e., $\max(0,x)$\\
$\displaystyle u(x)$ & Step function with value 1 when $x\geq0$ and value 0 otherwise\\
$\displaystyle \indicator\{\mathrm{condition}\}$ & is 1 if the condition is true, 0 otherwise\\
$\displaystyle \Phi(x), \Phi^{-1}(\pi)$ & Standard Gaussian cdf, and the probit function, where $\pi\in(0,1)$\\
$\displaystyle \text{Negative binomial}(\alpha, x)$ & $\eta = \ln(x/(x+1/\alpha))$\\
$\displaystyle h(\eta)$ & Response function in GLMs\\
$\displaystyle g(\mu)$ & Link function in GLMs\\
\end{tabular}
\egroup
\index{Sigmoid function}
\index{Softplus}
\index{Norm}
\end{minipage}
Sometimes we use a function $f$ whose argument is a scalar but apply
it to a vector, matrix: $f(\vx)$, $f(\mX)$.
This denotes the application of $f$ to the
array element-wise. For example, if $\bC = \sigma(\bX)$, then $c_{ij} = \sigma(x_{ij})$
for all valid values of $i$ and  $j$.

\vspace{0.4in}
%\vspace{\notationgap}
\begin{minipage}{\textwidth}
\centerline{\bf Other General Notastions}
\bgroup
\def\arraystretch{1.5}
\begin{tabular}{cp{4.25in}}
$\displaystyle \triangleq$ & Equals by definition\\
$\displaystyle :=, \leftarrow $ & Equals by assignment \\
$\displaystyle \equiv $    & Equals by equivalence \\
$\displaystyle \pi $       & A probability value or 3.141592....\\
$\displaystyle e, \exp $   & 2.71828...
\end{tabular}
\egroup
\end{minipage}

\vspace{0.4in}
%\vspace{\notationgap}
\begin{minipage}{\textwidth}
\centerline{\bf Abbreviations}
\bgroup
\def\arraystretch{1.5}
\begin{tabular}{cp{4.25in}}
PD & Positive definite  \\
PSD & Positive semidefinite \\
MCMC & Markov chain Monte Carlo \\
i.i.d. & Independently and identically distributed \\
p.d.f., PDF & Probability density function \\
p.m.f., PMF & Probability mass function \\
LS, OLS & Ordinary least squares\\
%NG & Normal-Gamma distribution \\
%NIG & Normal-inverse-Gamma  distribution \\
%NIX & Normal-inverse-Chi-squared distribution\\
%TN & Truncated-normal distribution \\
%GTN & General-truncated-normal distribution\\
%RN & Rectified-normal distribution \\
IW & Inverse-Wishart distribution \\
NIW & Normal-inverse-Wishart distribution \\
ALS & Alternating least squares \\
GD & Gradient descent\\
SGD & Stochastic gradient descent \\
%MU & Multiplicative update \\
MSE & Mean squared error\\
MLE & Maximum likelihood estimator\\
%NMF & Nonnegative matrix factorization\\
CLT & Central limit theorem \\
CMT & Continuous mapping theorem \\
%ID & Interpolative decomposition\\
%IID & Intervened interpolative decomposition \\
QR & QR decomposition \\
SVD & Singular value decomposition\\
ANOVA & Analysis of variance\\
GLM & Generalized linear model\\
GLS & Generalized least squares\\
REF & Row echelon form\\
RREF & Reduced row echelon form\\

\end{tabular}
\egroup
\end{minipage}

%\vspace{0.4in}
%%\vspace{\notationgap}
%\begin{minipage}{\textwidth}
%\centerline{\bf Datasets and Distributions}
%\bgroup
%\def\arraystretch{1.5}
%\begin{tabular}{cp{4.25in}}
%$\displaystyle \pdata$ & The data generating distribution\\
%$\displaystyle \ptrain$ & The empirical distribution defined by the training set\\
%$\displaystyle \sX$ & A set of training examples\\
%$\displaystyle \vx^{(i)}$ & The $i$-th example (input) from a dataset\\
%$\displaystyle y^{(i)}\text{ or }\vy^{(i)}$ & The target associated with $\vx^{(i)}$ for supervised learning\\
%$\displaystyle \mX$ & The $m \times n$ matrix with input example $\vx^{(i)}$ in row $\mX_{i,:}$\\
%\end{tabular}
%\egroup
%\end{minipage}

\clearpage

%\typeout{END_CHAPTER "notation" \theabspage}

\mainmatter

\newpage 
%\pagenumbering{arabic}
\chapter{Introduction}\label{chapter_lsintroduction}
\begingroup
\hypersetup{
linkcolor=structurecolor,
linktoc=page,  % page: only the page will be colored; section, all, none etc
}
\minitoc \newpage
\endgroup
\section{Introduction and Background}
\lettrine{\color{caligraphcolor}T}
This book is meant to provide an introduction to linear models and  their underlying theories. Our goal is to give a rigorous introduction to the readers with prior exposure to ordinary least squares. 
While machine learning often deals with nonlinear relationships, including those explored in deep learning with intricate layers demanding substantial computation, many algorithms are rooted in simple linear models.

The exposition approaches linear models from various perspectives, elucidating their properties and associated theories. 
In regression problems, the primary tool is the least squares approximation, minimizing the sum of squared errors. 
This is a natural choice when we're interested in finding the regression function, which minimizes the corresponding expected squared error.

This book is primarily a summary of purpose, emphasizing the  significance of important theories behind linear models, e.g., distribution theory, minimum variance estimator. 
We begin by presenting ordinary least squares from various distinct  points of view, upon which we disturb the model with random noise and Gaussian noise. 
The introduction of Gaussian noise establishes a likelihood, leading to the derivation of a maximum likelihood estimator and the development of distribution theories related to this Gaussian disturbance, which will help us answer various questions and introduce related applications. 
The subsequent proof establishes that least squares is the best unbiased linear model in terms of mean squared error, and notably, it approaches the theoretical limit. 
The exploration extends to linear models within a Bayesian framework and a generalized linear model framework.
The mathematical prerequisites are a first course in linear algebra and  statistics. Beyond these basic requirements, the content is self-contained, featuring rigorous proofs throughout.

Linear models play a central role in machine learning, particularly as the concatenation of simple linear models has led to the development of intricate nonlinear models like neural networks.  
The sole aim of this book is to give a self-contained introduction to concepts and mathematical tools in theory behind linear models and rigorous analysis in order to seamlessly introduce linear model methods and their applications in subsequent sections. 
It is acknowledged, however, that the book cannot comprehensively cover all valuable and interesting results related to linear models. Due to constraints, topics like the separate analysis of LASSO and ridge regression are not exhaustively discussed here. 
Interested readers are directed to relevant literature in the field of linear models for more in-depth exploration. 
Some excellent examples include \citet{strang1993introduction, panaretos2016statistics, hoff2009first, strang2021every, beck2014introduction, jackson2024glm}.

In the remainder of this chapter, we briefly introduce and review some basic notation and concepts from mathematics. Additional definitions will be introduced as needed throughout the text for clarity.

\section{Linear Algebra}

In all cases, scalars will be denoted in a non-bold font possibly with subscripts (e.g., $a$, $\alpha$, $\alpha_i$). We will use \textbf{boldface} lowercase letters possibly with subscripts to denote vectors (e.g., $\bmu$, $\bx$, $\bx_n$, $\bz$) and
\textbf{boldface} uppercase letters possibly with subscripts to denote matrices (e.g., $\bX$, $\bL_j$). The $i$-th element of a vector $\bz$ will be denoted by $z_i$ in non-bold font.
In the meantime, the \textit{normal fonts} of scalars denote  \textbf{random variables} (e.g., $\textnormal{a}$ and $\textnormal{b}_1$ are random variables, while italics $a$ and $b_1$ are scalars); 
the normal fonts of \textbf{boldface} lowercase letters, possibly with subscripts, denote \textbf{random vectors} (e.g., $\rva$ and $\rvb_1$ are random vectors, while italics $\ba$ and $\bb_1$ are vectors); 
and the normal fonts of \textbf{boldface} uppercase letters, possibly with subscripts, denote \textbf{random matrices} (e.g., $\rmA$ and $\rmB_1$ are random matrices, while italics $\bA$ and $\bB_1$ are matrices).

Subarrays are formed by fixing a subset of indices of a matrix.
The element located in the $i$-th row and $j$-th column of a matrix $\bX$ (i.e., the $(i,j)$ entry) is denoted by $x_{ij}$; in this case, $\bX\in\real^{n\times p}$ can be denoted as $\bX=\{x_{ij}\}_{i,j=1}^{n,p}=[x_{ij}]$.
Furthermore, it will be helpful to utilize the \textbf{Matlab-style notation}, the $i$-th row to the $j$-th row and the $k$-th column to the $m$-th column submatrix of the matrix $\bX$ will be denoted by $\bX_{i:j,k:m} \equiv \bX[i:j,k:m]$. A colon is used to indicate all elements of a dimension, e.g., $\bX_{:,k:m} \equiv\bX[:,k:m]$ denotes the $k$-th column to the $m$-th column of the matrix $\bX$, and $\bX_{:,k}\equiv\bX[:,k]$ denotes the $k$-th column of $\bX$. 
Alternatively, the $k$-th column of $\bX$ may be denoted more compactly by $\bx_k$; and the $k$-th row of $\bX$ can be denoted as $\bx^{(k)}$.

When the index is not continuous, given ordered subindex sets $\sI$ and $\sJ$, $\bX[\sI, \sJ]$ denotes the submatrix of $\bX$ obtained by extracting the rows and columns of $\bX$ indexed by $\sI$ and $\sJ$, respectively; and $\bX[:, \sJ]$ denotes the submatrix of $\bX$ obtained by extracting the columns of $\bX$ indexed by $\sJ$, where again the colon operator implies all indices.

\index{Matlab notation}
\begin{definition}[Matlab notation]\label{definition:matlabnotation}
Suppose $\bX\in \real^{n\times p}$, and $\sI=\{i_1, i_2, \ldots, i_k\}$ and $\sJ=\{j_1, j_2, \ldots, j_l\}$ are two index vectors. 
Then $\bX[\sI,\sJ]$ denotes the $k\times l$ submatrix
$$
\bX[\sI,\sJ]=
\begin{bmatrix}
x_{i_1,j_1} & x_{i_1,j_2} &\ldots & x_{i_1,j_l}\\
x_{i_2,j_1} & x_{i_2,j_2} &\ldots & x_{i_2,j_l}\\
\vdots & \vdots&\ddots & \vdots\\
x_{i_k,j_1} & x_{i_k,j_2} &\ldots & x_{i_k,j_l}\\
\end{bmatrix}.
$$
Whilst, $\bX[\sI,:]$ denotes the $k\times p$ submatrix, and $\bX[:,\sJ]$ denotes the $n\times l$ submatrix analogously.
%We note that it does not matter whether the index vectors $\sI$ and $\sJ$ are row vectors or column vectors. It matters which axis they index (rows of $\bX$ or columns of $\bX$). 
We should also notice that the range of the index satisfies:
$$
\left\{
\begin{aligned}
0&\leq \min(\sI) \leq \max(\sI)\leq n;\\
0&\leq \min(\sJ) \leq \max(\sJ)\leq p.
\end{aligned}
\right.
$$
\end{definition}

And in all cases, vectors are formulated in a column rather than in a row. A row vector will be denoted by a transpose of a column vector, such as $\bx^\top$. A specific column vector with values is separated  by the semicolons  $``;"$, e.g., 
$$\bx=[1;2;3] \qquad \text{(column vector)}
$$ 
is a column vector in $\real^3$. Similarly, a specific row vector with values is separated by the comma $``,"$, e.g., 
$$\by=[1,2,3]\qquad \text{(row vector)}
$$ 
is a row vector with 3 values. 
Alternatively, a column vector can also be written as the transpose of a row vector. For instance, $\by=[1,2,3]^\top$ is a column vector.

The transpose of a matrix $\bX$ will be denoted by $\bX^\top$,
and its inverse will be denoted by $\bX^{-1}$ . We will denote the $p \times p$ identity matrix by $\bI_p$. A vector or matrix of all zeros will be denoted by a \textbf{boldface} zero $\bzero$, whose size should be clear from context; or we denote $\bzero_p$ to be the vector of all zeros with $p$ entries.
Similarly, a vector or matrix of all ones will be denoted by a \textbf{boldface} one $\bone$, whose size is clear from  context; or we denote $\bone_p$ to be the vector of all ones with $p$ entries.
Subscripts are often omitted when the dimensions are evident from the context.

\index{Eigenvalue}
\index{Eigenvector}
\begin{definition}[Eigenvalue, Eigenvector]
Given any vector space $\sF$ and any linear map $\bX: \sF \rightarrow \sF$ (or simply a real matrix $\bX\in\real^{n\times n}$), a scalar $\lambda \in \sK$ is called a \textit{(right) eigenvalue, or proper value, or characteristic value} of $\bX$, if there exists  a nonzero vector $\bu \in \sF$ such that
\begin{equation*}
\bX \bu = \lambda \bu.
\end{equation*}
And $\bu$ is called a \textit{(right) eigenvector} of $\bX$ associated with $\lambda$.

On the other hand, $\kappa$ is referred to as a \textit{left eigenvalue} if there exists a nonzero vector $\bv\in \sF$ such that 
$$
\bv^\top\bX = \kappa \bv^\top.
$$
And $\bv$ is called a \textit{left eigenvector} of $\bX$ associated with $\kappa$.

When it is clear from the context, we will simply use the term ``eigenvalue/eigenvector" instead of ``right eigenvalue/eigenvector."

\end{definition}
For simplicity, we focus only on real-valued matrices unless otherwise specified. Unless explicitly stated otherwise, all eigenvalues discussed are assumed to be real as well.  

In simple terms, an eigenvector $\bu$ of a matrix $\bX$ represents a direction that remains unchanged when transformed into the coordinate system defined by the columns of $\bX$.
In fact, real-valued matrices can have complex eigenvalues. However, all eigenvalues of symmetric matrices are guaranteed to be real (see Theorem~\ref{theorem:spectral_theorem}).

\index{Spectrum}
\index{Spectral radius}
\begin{definition}[Spectrum and Spectral Radius]\label{definition:spectrum}
The set of all eigenvalues of a matrix $\bX$ is called the \textit{spectrum} of  $\bX$, and is denoted by $\Lambda(\bX)$. 
%The set of all eigenvalues of $\bX$ including their algebraic multiplicity (Definition~\ref{definition:eigen_multipli}) is called the \textit{multispectrum} of $\bX$ and is denoted by $\widehat{\Lambda}(\bX)$. It holds that $\Lambda(\bX)\subseteq \widehat{\Lambda}(\bX)$.
The largest magnitude  among the eigenvalues is known as the \textit{spectral radius} of $\bX$, denoted by $\rho(\bX)$:
$
\rho(\bX) = \mathop{\max}_{\lambda\in \Lambda(\bX)}  \abs{\lambda}.
$
\end{definition}

Moreover, the tuple $(\lambda, \bu)$  is referred to as an \textit{eigenpair}. Intuitively, the above definitions mean that multiplying matrix $\bX$ by the vector $\bu$ results in a new vector that is in the same direction as $\bu$, but only scaled by a factor $\lambda$. For any eigenvector $\bu$, we can scale it by a scalar $s$ such that $s\bu$ is still an eigenvector of $\bX$. That's why we say that the eigenvector is an eigenvector of $\bX$ associated with the eigenvalue $\lambda$. To avoid ambiguity, we usually assume that the eigenvector is normalized to have length one and the first entry is positive (or negative) since both $\bu$ and $-\bu$ are eigenvectors. 

In linear algebra, every vector space has a basis, and every vector in that space can be expressed as a linear combination of the basis vectors. Based on this idea, we define the span and dimension of a subspace using the concept of a basis.

\index{Subspace}
\begin{definition}[Subspace]
A nonempty subset $\mathcalV$ of $\real^n$ is called a subspace if $x\ba+y\ba\in \mathcalV$ for every $\ba,\bb\in \mathcalV$ and every $x,y\in \real$.
\end{definition}

\index{Span}
\begin{definition}[Span]
If every vector $\bv$ in a subspace $\mathcalV$ can be expressed as a linear combination of $\{\bx_1, \bx_2, \ldots,$ $\bx_n\}$, then we say that these vectors \textit{span} the subspace $\mathcalV$.
\end{definition}

\index{Linearly independent}
The concept of linear independence of a set of vectors is central to linear algebra. Two equivalent definitions are given below.
\begin{definition}[Linearly independent]
A set of vectors $\{\bx_1, \bx_2, \ldots, \bx_n\}$ is said to be \textit{linearly independent} if there is no combination can get $a_1\bx_1+a_2\bx_2+\ldots+a_n\bx_n=0$ except all $a_i$'s are zero. An equivalent definition is that $\bx_1\neq \bzero$, and for every $k>1$, the vector $\bx_k$ does not belong to the span of $\{\bx_1, \bx_2, \ldots, \bx_{k-1}\}$.
\end{definition}

\index{Basis}
\index{Dimension}
\begin{definition}[Basis and dimension]
A set of vectors $\{\bx_1, \bx_2, \ldots, \bx_n\}$ is called a \textit{basis} of a subspace $\mathcalV$ if they are linearly independent, and they span $\mathcalV$. Every basis of a given subspace contains the same number of vectors, and the number of vectors in any basis is called the \textit{dimension} of the subspace $\mathcalV$. By convention, the trivial subspace $\{\bzero\}$ is said to have dimension zero. 
Furthermore, every subspace with a nonzero dimension has an orthogonal basis; in other words, the basis of a subspace can be chosen orthogonal (Definition~\ref{definition:orthogn_mat}).
\end{definition}

\index{Column space}
\begin{definition}[Column space (range)]
Let $\bX$ be an $n \times p$ real matrix. The \textit{column space (or range)} of $\bX$ is defined as  the set of all vectors that can be expressed as a linear combination of its columns:
\begin{equation*}
\cspace (\bX) = \{ \bv\in \mathbb{R}^n: \exists \bu \in \mathbb{R}^p, \, \bv = \bX \bu \}.
\end{equation*}
Similarly, the row space of $\bX$ is the set of all vectors spanned by its rows. Equivalently, it is the column space of $\bX^\top$:
\begin{equation*}
\cspace (\bX^\top) = \{ \bu\in \mathbb{R}^p: \exists \bv \in \mathbb{R}^n, \, \bu = \bX^\top \bv \}.
\end{equation*}
\end{definition}

\index{Null space}
\index{Left null space}
\index{Right null space}
\begin{definition}[Null space (nullspace, kernel)]\label{definition:nullspace}
Let $\bX$ be an $n \times p$ real matrix. The \textit{null space (or kernel, or nullspace)} of $\bX$ is defined as  the set:
\begin{equation*}
\nspace (\bX) = \{\bv \in \mathbb{R}^p:  \, \bX \bv = \bzero \}.
\end{equation*}
In some cases, the null space of $\bX$ is also referred to as the \textit{right null space} of $\bX$.
And the null space of $\bX^\top$ is defined as 	
\begin{equation*}
\nspace (\bX^\top) = \{\bu \in \mathbb{R}^n:  \, \bX^\top \bu = \bzero \}.
\end{equation*}
Similarly, the null space of $\bX^\top$ is also referred to as the \textit{left null space} of $\bX$.
\end{definition}

Both the column space of $\bX$ and the null space of $\bX^\top$ are subspaces of $\real^n$. In fact, every vector in $\nspace(\bX^\top)$ is orthogonal  to $\cspace(\bX)$ and vice versa. Similarly, every vector in $\nspace(\bX)$ is also orthogonal  to $\cspace(\bX^\top)$ and vice versa.

\index{Rank}
\index{Dimension}
\begin{definition}[Rank]\label{definition:rank}
The $rank$ of a matrix $\bX\in \real^{n\times p}$ is defined as the dimension of its column space. That is, the rank of $\bX$ is equal to the maximum number of linearly independent columns of $\bX$, and is also the maximum number of linearly independent rows of $\bX$. The matrix $\bX$ and its transpose $\bX^\top$ have the same rank. We say that $\bX$ has full rank if its rank is equal to $\min\{n,p\}$.  Specifically, given a vector $\bu \in \real^n$ and a vector $\bv \in \real^p$, then the $n\times p$ matrix $\bu\bv^\top$ obtained by the outer product of vectors is of rank 1. In short, the rank of a matrix is equal to:
\begin{itemize}
\item the number of linearly independent columns;
\item the number of linearly independent rows;
\item and remarkably, these are always the same (see Lemma~\ref{lemma:equal-dimension-rank}).
\end{itemize}
\end{definition}

\index{Orthogonal complement}
\begin{definition}[Orthogonal complement in general]
The \textit{orthogonal complement} $\mathcalV^\perp\subseteq\real^n$ of a subspace $\mathcalV\subseteq\real^n$ consists of all vectors that are perpendicular to every vector in $\mathcalV$. That is,
$$
\mathcalV^\perp = \{\bv\in\real^n: \bv^\top\bu=0, \ \forall\, \bu\in \mathcalV  \}.
$$
These two subspaces are disjoint and together span the entire space $\real^n$. 
The dimensions of $\mathcalV$ and $\mathcalV^\perp$ add up to the dimension of the entire space: $\dim(\mathcalV)+\dim(\mathcalV^\perp)=n$. Furthermore, $(\mathcalV^\perp)^\perp=\mathcalV$.
\end{definition}

\begin{definition}[Orthogonal complement of column space]\label{definition:ortho_comp_col}
Let $\bX$ be an $n \times p$ real matrix. The orthogonal complement of the column space $\cspace(\bX)$, denoted by $\cspace^{\bot}(\bX)$, is the subspace defined as:
\begin{equation*}
\begin{aligned}
\cspace^{\bot}(\bX) &= \{\bv\in \mathbb{R}^n: \, \bv^\top \bX \bu=\bzero, \, \forall\, \bu \in \mathbb{R}^p \} \\
&=\{\bv\in \mathbb{R}^n: \, \bv^\top \bw = \bzero, \, \forall\, \bw \in \cspace(\bX) \}.
\end{aligned}
\end{equation*}
\end{definition}
Then we have the \textit{four fundamental spaces} for any matrix $\bX\in \real^{n\times p}$ with rank $r$:
\begin{itemize}
\item  $\cspace(\bX)$: Column space of $\bX$, i.e., linear combinations of columns with dimension $r$.
\item  $\nspace(\bX)$: (Right) null space of $\bX$, i.e., all $\bu$ satisfying $\bX\bu=\bzero$ with dimension $p-r$.
\item  $\cspace(\bX^\top)$: Row space of $\bX$, i.e., linear combinations of rows with dimension $r$.
\item  $\nspace(\bX^\top)$: Left null space of $\bX$, i.e., all $\bv$ satisfying $\bX^\top \bv=\bzero$ with dimension $n-r$. 
\end{itemize}
Furthermore, $\nspace(\bX)$ is the orthogonal complement of $\cspace(\bX^\top)$, and $\cspace(\bX)$ is the orthogonal complement of $\nspace(\bX^\top)$. The proof is further discussed in Theorem~\ref{theorem:fundamental-linear-algebra}.

\index{Rank}
\index{Dimension}
We establish the equivalence stated in Definition~\ref{definition:rank}.
\begin{lemma}[Dimension of column space and row space]\label{lemma:equal-dimension-rank}
The dimension of the column space of a matrix $\bX\in \real^{n\times p}$ is equal to the dimension of its
row space, i.e., the row rank and the column rank of a matrix $\bX$ are equal.
\end{lemma}
\begin{proof}[of Lemma~\ref{lemma:equal-dimension-rank}]
We first notice that the null space of $\bX$ is orthogonal complementary to the row space of $\bX$: $\nspace(\bX) \bot \cspace(\bX^\top)$ (where the row space of $\bX$ is equivalent to the column space of $\bX^\top$). 
That is, vectors in the null space of $\bX$ are orthogonal to vectors in the row space of $\bX$. To see this, suppose $\bX=[\bx_1^\top; \bx_2^\top; \ldots; \bx_n^\top]$ is the row partition of $\bX$. For any vector $\bbeta\in \nspace(\bX)$, we have $\bX\bbeta = \bzero$, that is, $[\bx_1^\top\bbeta; \bx_2^\top\bbeta; \ldots; \bx_n^\top\bbeta]=\bzero$. 
And since the row space of $\bX$ is spanned by $\{\bx_1^\top, \bx_2^\top, \ldots, \bx_n^\top\}$, thus, $\bbeta$ is perpendicular to any vectors from $\cspace(\bX^\top)$. This indicates $\nspace(\bX) \bot \cspace(\bX^\top)$.

Now, assuming the dimension of the row space of $\bX$ is $r$,  \textcolor{mylightbluetext}{let $\br_1, \br_2, \ldots, \br_r$ be a set of vectors in $\real^p$ and form a basis for the row space}. 
Consequently, the $r$ vectors $\bX\br_1, \bX\br_2, \ldots, \bX\br_r$ are in the column space of $\bX$; furthermore, they are linearly independent. To see this, suppose we have a linear combination of the $r$ vectors: $\beta_1\bX\br_1 + \beta_2\bX\br_2+ \ldots+ \beta_r\bX\br_r=0$, that is, $\bX(\beta_1\br_1 + \beta_2\br_2+ \ldots+ \beta_r\br_r)=0$, and the vector $\bv=\beta_1\br_1 + \beta_2\br_2+ \ldots+ \beta_r\br_r$ is in null space of $\bX$. But since $\{\br_1, \br_2, \ldots, \br_r\}$ is a basis for the row space of $\bX$, $\bv$ is thus also in the row space of $\bX$. We have shown that vectors from null space of $\bX$ is perpendicular to vectors from row space of $\bX$, thus $\bv^\top\bv=0$ and $\beta_1=\beta_2=\ldots=\beta_r=0$. Then, \textcolor{mylightbluetext}{$\bX\br_1, \bX\br_2, \ldots, \bX\br_r$ are in the column space of $\bX$ and they are independent}. This means that the dimension of the column space of $\bX$ is larger than $r$. This result shows that \textbf{row rank of $\bX\leq $ column rank of $\bX$}. 

If we apply this process again to $\bX^\top$, we will have \textbf{column rank of $\bX\leq $ row rank of $\bX$}. This completes the proof. 
\end{proof}

From the previous proof, we can also conclude that if $\br_1, \br_2, \ldots, \br_r$ is a set of vectors in $\real^p$ that forms a basis for the row space of $\bX$, then \textcolor{mylightbluetext}{$\bX\br_1, \bX\br_2, \ldots, \bX\br_r$ forms a basis for the column space of $\bX$}. 
We formalize this result in the following lemma.

\index{Column basis}
\index{Row basis}
\begin{lemma}[Column basis from row basis]\label{lemma:column-basis-from-row-basis}
For any matrix $\bX\in \real^{n\times p}$, let $\{\br_1, \br_2, \ldots, \br_r\}$ be a set of vectors in $\real^p$, which forms a basis for the row space of $\bX$. Then, the set $\{\bX\br_1, \bX\br_2, \ldots, \bX\br_r\}$ is a basis for the column space of $\bX$.
\end{lemma}

\index{Fundamental spaces}
\index{Rank-nullity}
\index{Orthogonal matrix}
\begin{definition}[Orthogonal matrix, semi-orthogonal matrix]\label{definition:orthogn_mat}
A real square matrix $\bQ\in \real^{n\times n}$ is an \textit{orthogonal} matrix if the inverse of $\bQ$ equals its transpose, that is $\bQ^{-1}=\bQ^\top$ and $\bQ\bQ^\top = \bQ^\top\bQ = \bI$. In other words, suppose $\bQ=[\bq_1, \bq_2, \ldots, \bq_n]$, where $\bq_i \in \real^n$ for all $i \in \{1, 2, \ldots, n\}$, then $\bq_i^\top \bq_j = \delta(i,j)$ with $\delta(i,j)$ being the Kronecker delta function. 
For any vector $\bx$, the orthogonal matrix will preserve the length: $\normtwo{\bQ\bx} = \normtwo{\bx}$, where $\normtwo{\cdot}$ denotes the $\ell_2$ norm (Definition~\ref{definition:vec_l2_norm}).
Note that, since the orthogonal matrix $\bQ$ contains unit-length columns, the columns are mutually orthogonormal. However, the terminology of \textit{orthogonormal matrix} is \textbf{not} used due to historical convention; instead, the term orthogonal matrix is standard even though it implies orthonormal columns.

On the other hand, if $\bQ$ contains only $\gamma<n$ of these columns, then $\bQ^\top\bQ = \bI_\gamma$ stills holds, where $\bI_\gamma$ is the $\gamma\times \gamma$ identity matrix. But $\bQ\bQ^\top=\bI$ will not be true. 
In this case, $\bQ$ referred to as a \textit{semi-orthogonal} matrix.
\end{definition}

\index{Orthonormal basis}
\index{Orthogonal matrix}
\index{Orthogonal vs orthonormal}
%\subsection*{Orthogonal vs Orthonormal}
The vectors $\bq_1, \bq_2, \ldots, \bq_\gamma\in \real^n$ are \textit{mutually orthogonal} when their dot products $\bq_i^\top\bq_j$ are zero whenever $i \neq j$. When each vector is divided by its length, the vectors become orthogonal unit vectors. Then the vectors $\bq_1, \bq_2, \ldots, \bq_\gamma$ are called \textit{mutually orthonormal}. We usually put the orthonormal vectors into a matrix $\bQ$.

When $n\neq \gamma$: the matrix $\bQ$ is easy to work with because $\bQ^\top\bQ=\bI \in \real^{\gamma\times \gamma}$.

When $n= \gamma$: the matrix $\bQ$ is square, $\bQ^\top\bQ=\bI$ means that $\bQ^\top=\bQ^{-1}$, i.e., the transpose of $\bQ$ is the inverse of $\bQ$. Then we also have $\bQ\bQ^\top=\bI$, i.e., $\bQ^\top$ is the two-sided inverse of $\bQ$. We call this $\bQ$ an \textit{orthogonal matrix}. 

\index{Idempotent matrix}
\begin{definition}[Idempotent Matrix]\label{definition:Idempotent_mat}
	A matrix $\bX\in \real^{n\times n}$ is  called \textit{idempotent} if $\bX^2 = \bX$.
\end{definition}

\index{Permutation matrix}
\begin{definition}[Permutation matrix]\label{definition:permutation-matrix}
A \textit{permutation matrix} $\bP\in \real^{n\times n}$ is a square binary matrix that has exactly one entry of 1 in each row and each column, and 0's elsewhere. 
\paragraph{Row point.} That is, the permutation matrix $\bP$ has the rows of the identity $\bI$ in any order, and the order decides the sequence of the row permutation. If we want to permute the rows of a matrix $\bX$, we multiply on the left $\bP\bX$. 
\paragraph{Column point.} Or, equivalently, the permutation matrix $\bP$ has the columns of the identity $\bI$ in any order, and the order decides the sequence of the column permutation. To apply a column permutation to $\bX$, we multiply on the right $\bX\bP$.
\end{definition}

The permutation matrix $\bP$ can be more efficiently represented via a set $\sJ \in \integer_+^n$ of indices such that $\bP = \bI[:, \sJ]$, where $\bI$ is the $n\times n$ identity matrix. And notably, the elements in vector $\sJ$ sum to $1+2+\ldots+n= \frac{n^2+n}{2}$.

\begin{example}[Permutation]
Suppose 
$$\bX=\begin{bmatrix}
1 & 2&3\\
4&5&6\\
7&8&9
\end{bmatrix}
\qquad \text{and} \qquad
\bP=\begin{bmatrix}
&1&\\
&&1\\
1&&
\end{bmatrix}.
$$
The row and column permutations are given by 
$$
\bP\bX = \begin{bmatrix}
4&5&6\\
7&8&9\\
1 & 2&3\\
\end{bmatrix}
\qquad \text{and}\qquad 
\bX\bP = \begin{bmatrix}
3 & 1 & 2 \\
6 & 4 & 5\\
9 & 7 & 8
\end{bmatrix},
$$
respectively,
where the order of the rows of $\bX$ appearing in $\bP\bX$ matches the order of the rows of $\bI$ in $\bP$,
and the order of the columns of $\bX$ appearing in $\bX\bP$ matches the order of the columns of $\bI$ in $\bP$. 
\end{example}

\index{Determinant}
\begin{definition}[Determinant: Laplace Expansion by Minors]\label{definition:determinant}
Let $\bX\in\real^{n\times n}$ be any square matrix, and let $\bX_{ij}\in\real^{(n-1)\times (n-1)}$ denote the submatrix of $\bX$ obtained by deleting the $i$-th row and $j$-th column. 
The \textit{determinant} of $\bX$ can be computed recursively using the following equations:
\begin{equation}\label{equation:def_det}
\det(\bX)=\sum_{k=1}^{n} (-1)^{i+k} a_{ik}\det(\bX_{ik})=\sum_{k=1}^{n}(-1)^{k+j} a_{kj}\det(\bX_{kj}),
\end{equation}
where the first equation is the \textit{Laplace expansion by minors along row $i$}, and the second equation is the \textit{Laplace expansion by minors along column $j$}.
Equivalently, given a cardinality $r$, and consider an index set $\sJ\subseteq\{1,2,\ldots,n\}$ with cardinality $r$ ($\abs{\sJ}=r$) and its complementary set $\comple{\sJ}=\{1,2,\ldots,n\}\backslash \sJ$, we have:
$$
\begin{aligned}
\det(\bX) 
=\sum_{\sI} (-1)^{\gamma} \det(\bX[\sI,\sJ])\det(\bX[\comple{\sI}, \comple{\sJ}])
=\sum_{\sI} (-1)^{\gamma} \det(\bX[\sJ,\sI])\det(\bX[\comple{\sJ}, \comple{\sI}]),
\end{aligned}
$$
where $\gamma=\sum_{i\in \sI} i +\sum_{j\in \sJ}j$, and the sum is taken over all the index sets $\sI\subseteq\{1,2,\ldots,n\}$ with cardinality $r$.
When $r=1$, this reduces to \eqref{equation:def_det}.
\end{definition}

The determinant of a square matrix maps the matrix into a scalar value.
In the case of a $2\times 2$ matrix, the determinant represents the area of the parallelogram spanned by the column vectors of the matrix. It is positive if the orientation is counterclockwise and negative if clockwise.
For a $3\times 3$ matrix, the determinant corresponds to the volume of the parallelepiped formed by the three column vectors of the matrix. Again, the sign indicates whether the orientation is preserved or reversed.
For a matrix representing a linear transformation in $n$-dimensional space, the absolute value of the determinant gives the factor by which the volume changes under this transformation. If the determinant is positive, the orientation (or handedness) of the basis is preserved; if negative, it is reversed.
A matrix is invertible if and only if its determinant is nonzero. This means that the transformation does not collapse the space into a lower dimension or a single point, which would happen if the determinant were zero.
The determinant of a matrix is equal to the product of its eigenvalues. This means that the determinant reflects the combined effect of all the stretching factors applied by the matrix to the eigenvectors.
We provide a few properties of the determinant.
\begin{lemma}[Properties of determinant]\label{lemma:determinant-intermezzo}
We have the following properties for determinant of matrices:
\begin{itemize}
\item  The determinant of the product of two matrices is $\det(\bX\bY)=\det (\bX)\det(\bY)$;

\item The determinant of the transpose is $\det(\bX^\top) = \det(\bX)$;

\item Suppose matrix $\bX$ has an eigenvalue $\lambda$, then $\det(\bX-\lambda\bI) =0$;

\item Determinant of any identity matrix is $1$;

\item Determinant of an orthogonal matrix $\bQ$: 
$$
\det(\bQ) = \det(\bQ^\top) = \pm 1, \qquad \text{since  } \det(\bQ^\top)\det(\bQ)=\det(\bQ^\top\bQ)=\det(\bI)=1;
$$

\item Given any square matrix $\bX$ and orthogonal matrix $\bQ$, we have 
$$
\det(\bX) = \det(\bQ^\top) \det(\bX)\det(\bQ) =\det(\bQ^\top\bX\bQ);
$$
\item Suppose $\bX\in\real^{n\times n}$, then $\det(-\bX) = (-1)^n \det(\bX)$.
\end{itemize}
\end{lemma}

\index{Positive definite}\index{Positive semidefinite}
Positive definiteness or positive semidefiniteness is one of the highest accolades to which a matrix can aspire. 
In Section~\ref{section:choleskydecomp}, we will introduce the Cholesky decomposition, which applies specifically to positive definite matrices. We begin with the following definition.
\begin{definition}[Positive definite and positive semidefinite]\label{definition:psd-pd-defini}
A matrix $\bX\in \real^{n\times n}$ is said to be \textit{positive definite (PD)} if $\bbeta^\top\bX\bbeta>0$ for all nonzero $\bbeta\in \real^n$, denoted by $\bX\succ \bzero$.
And a matrix $\bX\in \real^{n\times n}$ is called \textit{positive semidefinite (PSD)} if $\bbeta^\top\bX\bbeta \geq 0$ for all $\bbeta\in \real^n$, denoted by $\bX\succeq \bzero$. 
\footnote{
In this book, a positive definite or positive semidefinite matrix is always assumed to be symmetric. That is, the concepts of positive definiteness and semidefiniteness are meaningful only for symmetric matrices.
}
\footnote{A symmetric matrix $\bX\in\real^{n\times n}$ is called \textit{negative definite} (ND) if $\bbeta^\top\bX\bbeta<0$ for all nonzero $\bbeta\in\real^n$; 
a symmetric matrix $\bX\in\real^{n\times n}$ is called \textit{negative semidefinite} (NSD) if $\bbeta^\top\bX\bbeta\leq 0$ for all $\bbeta\in\real^n$;
and a symmetric matrix $\bX\in\real^{n\times n}$ is called \textit{indefinite} (ID) if there exist $\bbeta$ and $\balpha\in\real^n$ such that $\bbeta^\top\bX\bbeta<0$ and $\balpha^\top\bX\balpha>0$.
}
\end{definition}

We can show that a matrix $\bX$ is positive definite if and only if all of its eigenvalues are strictly positive. Similarly, $\bX$ is positive semidefinite if and only if all of its eigenvalues are nonnegative.

This result leads to the following theorem:
\begin{theoremHigh}[Eigenvalue characterization theorem]\label{theorem:eigen_charac}
A matrix $\bA$ is positive definite if and only if it contains only \textit{positive eigenvalues}. Similarly, a matrix $\bA$ is positive semidefinite if and only if it contains only  \textit{nonnegative eigenvalues}.~\footnote{The trace, determinant, and principal minors of a positive (semi)definite matrix is discussed in Problem~\ref{prob:tr_de_pd}.}
Moreover, we have the following implications:
\begin{itemize}
\item $\bA-\gamma\bI\succeq \bzero$ if and only if $\lambda_{\min}(\bA) \geq \gamma$;
\item  $\bA-\gamma\bI\succ \bzero$ if and only if $\lambda_{\min}(\bA) > \gamma$;
\item $\bA-\gamma\bI\preceq \bzero$ if and only if $\lambda_{\max}(\bA) \leq \gamma$;
\item $\bA-\gamma\bI\prec \bzero$ if and only if $\lambda_{\max}(\bA) < \gamma$;
\item $\lambda_{\min}(\bA)\bI\preceq \bA \preceq \lambda_{\max}(\bA)\bI$,
\end{itemize}
where $\lambda_{\min}(\bA)$ and $\lambda_{\max}(\bA)$ represent the minimum and maximum eigenvalues of $\bA$, respectively, and $\bB \prec \bC$ means $\bC-\bB$ is PSD.
\end{theoremHigh}
Given the eigenpair $(\lambda, \bx)$ of $\bA$, the forward implication can be shown that $\bx^\top\bA\bx=\lambda\bx^\top\bx>0$ such that $ \lambda=(\bx^\top\bA\bx)/(\bx^\top\bx)>0$ (resp. $\geq 0$) if $\bA$ is PD (resp. PSD).
The full proof of this equivalence can  be proved using the spectral theorem (Theorem~\ref{theorem:spectral_theorem}).
This theorem provides an alternative definition of positive definiteness and positive semidefiniteness in terms of the eigenvalues of the matrix, which is a fundamental property for the Cholesky decomposition.

\index{Positive semidefinite}
\begin{exercise}[Power of PSD]
Let $\bA$ be PSD. Show that  $\bA^k$ is also PSD for $k=1,2,\ldots$.
\end{exercise}

\index{Nonsingularity}
\index{Nonsingular matrix}
\index{Invertible matrix}
From an introductory course on linear algebra, we have the following remark on the equivalent claims of nonsingular matrices.
\begin{remark}[List of equivalence of nonsingularity for a matrix]\label{remark:equiva_nonsingular}
Given a square matrix $\bX\in \real^{n\times n}$, the following claims are equivalent:
\begin{itemize}
\item $\bX$ is nonsingular;~\footnote{The source of the name  is a result of the singular value decomposition (SVD).}
\item $\bX$ is invertible, i.e., $\bX^{-1}$ exists;
\item $\bX\bu=\by$ has a unique solution $\bu = \bX^{-1}\by$;
\item $\bX\bu = \bzero$ has a unique, trivial solution: $\bu=\bzero$;
\item Columns of $\bX$ are linearly independent;
\item Rows of $\bX$ are linearly independent;
\item $\det(\bX) \neq 0$; 
\item $\dim(\nspace(\bX))=0$;
\item $\nspace(\bX) = \{\bzero\}$, i.e., the null space is trivial;
\item $\cspace(\bX)=\cspace(\bX^\top) = \real^n$, i.e., the column space or row space span the entire $\real^n$;
\item $\bX$ has full rank $r=n$;
\item The reduced row echelon form is $\bR=\bI$;
\item $\bX^\top\bX$ is symmetric positive definite (PD);
\item $\bX$ has $n$ nonzero (positive) singular values;
\item All eigenvalues are nonzero.
\end{itemize}
\end{remark}
It is important to keep the above equivalence in mind. Additionally, the following remark presents equivalent statements that apply to singular matrices as well.\index{Singular matrix}
\begin{remark}[List of equivalence of singularity for a matrix]\label{remark:equiva_singular}
For a square matrix $\bX\in \real^{n\times n}$ with an eigenpair $(\lambda, \bu)$, the following claims are equivalent:
\begin{itemize}
\item $(\bX-\lambda\bI)$ is singular;
\item $(\bX-\lambda\bI)$ is not invertible;
\item $(\bX-\lambda\bI)\bv = \bzero$ has nonzero $\bv\neq \bzero$ solutions, and $\bv=\bu$ is one of such solutions;
\item $(\bX-\lambda\bI)$ has linearly dependent columns;
\item $\det(\bX-\lambda\bI) = 0$; 
\item $\dim(\nspace(\bX-\lambda\bI))>0$;
\item Null space of $(\bX-\lambda\bI)$ is nontrivial;
\item Columns of $(\bX-\lambda\bI)$ are linearly dependent;
\item Rows of $(\bX-\lambda\bI)$ are linearly dependent;
\item $(\bX-\lambda\bI)$ has rank $r<n$;
\item Dimension of column space = dimension of row space = $r<n$;
\item $(\bX-\lambda\bI)^\top(\bX-\lambda\bI)$ is symmetric semidefinite;
%\item $\bX^\top\bX$ is symmetric semidefinite;
\item $(\bX-\lambda\bI)$ has $r<n$ nonzero (positive) singular values;
\item Zero is an eigenvalue of $(\bX-\lambda\bI)$.
\end{itemize}
\end{remark}

Given a vector or a matrix, its norm should satisfy the following three criteria.
\begin{definition}[Vector norm and matrix nrom\index{Matrix norm}\index{Vector norm}]\label{definition:matrix-norm}
Given a norm $\norm{\cdot}$ defined on either vectors or matrices, for any matrix $\bX \in \real^{n\times p}$ and any vector $\bx \in \real^{n}$, the following three properties must be satisfied:
\begin{itemize}
\item \textit{Nonnegativity}. $\norm{\bX} \geq 0$ or $\norm{\bx}\geq 0$, and the equality is obtained if and only if $\bX=\bzero $ or $\bx=\bzero$. 
\item \textit{Positive homogeneity}. $\norm{\lambda \bX} = \abs{\lambda} \cdot \norm{\bX}$ or $\norm{\lambda \bx} = \abs{\lambda} \cdot \norm{\bx}$ for any $\lambda \in \real$.
\item \textit{Triangle inequality}. $\norm{\bX+\bY} \leq \norm{\bX}+\norm{\bY}$, or $\norm{\bx+\by} \leq \norm{\bx}+\norm{\by}$ for any matrices $\bX, \bY\in \real^{n\times p}$ or vectors $\bx,\by\in \real^n$.
\end{itemize}
\end{definition}
\index{Vector norm}
\index{Matrix norm}

Based on this definition of norms, we can now define specific types of vector norms, namely the $\ell_1$, $\ell_2$, and $\ell_\infty$ norms for a vector.
\begin{definition}[Vector $\ell_1, \ell_2, \ell_\infty$, $\ell_p$ norms]\label{definition:vec_l2_norm}
For a vector $\bx\in\real^n$, the \textit{$\ell_2$ vector norm} is defined as $\normtwo{\bx} = \sqrt{x_1^2+x_2^2+\ldots+x_n^2}$.
Similarly, the $\ell_1$ \textit{norm} can be obtained by 
$
\norm{\bx}_1 = \sum_{i=1}^{n} |x_i| .
$
And the $\ell_\infty$ \textit{norm} can be obtained by 
$
\norm{\bx}_\infty = \mathop{\max}_{i=1,2,\ldots,n} |x_i| .
$
More generally, the $\ell_p$ norm is defined as $\normp{\bx}=\sqrt[p]{ \sum_{i=1}^{n}|x_i|^p  }$ for $p\geq 1$.
\end{definition}

\index{Dual norm}
\index{H\"older's inequality}
\paragrapharrow{Dual norm.}
Consider the $\ell_p$ vector norm. From \textit{\holders inequality}, we have 
$
\bx^\top\by \leq \norm{\bx}_p \norm{\by}_q,
$
where $p,q>1$ satisfy $\frac{1}{p}+\frac{1}{q}=1$, and $\bx,\by\in \real^n$. Equality holds if the two sequences $\{\abs{x_i}^p\}$ and $\{\abs{y_i}^q\}$ are linearly dependent. This implies
\begin{equation}\label{equation:dual_norm_equa}
\mathop{\max}_{\norm{\by}_q=1} \bx^\top\by = \norm{\bx}_p.
\end{equation}
For this reason, $\norm{\cdot}_q$ is called the \textit{dual norm} of $\norm{\cdot}_p$.
On the other hand, for each $\bx\in \real^n$ with $\norm{\bx}_p=1$, there exists a vector $\by\in \real^n$ such that $\norm{\by}_q=1$ and $\bx^\top\by=1$.
Notably, the $\ell_2$ norm is self-dual, while the $\ell_1$ and $\ell_\infty$ norms are dual to each other.

\begin{definition}[Set of primal counterparts]\label{definition:set_primal}
Let $ \norm{\cdot} $ be any norm on $ \real^n $. Then the \textit{set of primal counterparts of $\ba$} is defined as 
\begin{equation}
\Lambda_{\ba} = \argmax_{\bu \in \real^n} \left\{ \innerproduct{\ba, \bu} \mid  \norm{\bu} \leq 1 \right\}.
\end{equation}
That is, $\innerproduct{\ba, \ba^\dagger} = \norm{\ba}_*$ for any $\ba^\dagger\in \Lambda_{\ba}$, where $\norm{\cdot}_{*}$ denotes the dual norm.
It follows that 
\begin{enumerate}[(i)]
\item If $\ba \neq \bzero$, then $\norm{\ba^\dagger} = 1$ for any $\ba^\dagger \in \Lambda_{\ba}$.
\item If $\ba = \bzero$, then $\Lambda_{\ba} = \{\bx\in\real^n\mid  \norm{\bx} \leq 1\}$.
%\item $\innerproduct{\ba, \ba^\dagger} = \norm{\ba}_*$ for any $\ba^\dagger \in \Lambda_{\ba}$.
\end{enumerate}
\end{definition}

\begin{example}[Set of Primal Counterparts]\label{example:set_primal_count}
A few examples for the sets of primal counterparts are shown below:
\begin{itemize}
\item If the norm is the  $\ell_2$ norm, then for any $\ba \neq \bzero$,
$
\Lambda_{\ba} = \left\{ {\ba}/{\normtwo{\ba}} \right\}.
$

\item If the norm is the  $\ell_1$ norm, then for any $\ba \neq \bzero$,
$$
\Lambda_{\ba} = \left\{ \sum_{i \in \sI(\ba)} \lambda_i \sign(a_i) \be_i \mid \sum_{i \in \sI(\ba)} \lambda_i = 1, \lambda_j \geq 0, j \in \sI(\ba) \right\},
$$
where $\sI(\ba) \triangleq \argmax_{i=1,2,\ldots,n} |a_i|$.

\item If the norm is the  $\ell_\infty$ norm, then for any $\ba \neq \bzero$,
$$
\Lambda_{\ba} = \left\{ \bx \in \real^n \mid x_i = \sign(a_i), i \in \sI_{\neq}(\ba), |x_j| \leq 1, j \in \sI_0(\ba) \right\},
$$
where
$
\sI_{\neq}(\ba) \triangleq \left\{ i \in \{1, 2, \ldots, n\} \mid a_i \neq 0 \right\} $ and $ \sI_0(\ba) \triangleq \left\{ i \in \{1, 2, \ldots, n\} \mid a_i = 0 \right\}.
$
\end{itemize}
These examples play a crucial role in the development of non-Euclidean gradient descent methods, which will be discussed in Sections~\ref{section:als-gradie-descent-taylor}.
\end{example}

Given a specific norm definition, we introduce the concepts of an open ball and a closed ball as follows:
\begin{definition}[Open ball, closed ball]\label{definition:open_closed_ball}
Let $\norm{\cdot}_p: \real^n\rightarrow \real_+$ be the $\ell_p$ norm function. The \textit{open ball} centered at $\bc\in\real^n$ with radius $r$  is defined as 
$$
\sB_p(\bc, r) \triangleq \{\bx\in\real^n\mid  \norm{\bx-\bc}_p <r\}.
$$
Similarly, the \textit{closed ball} centered at $\bc\in\real^n$ with radius $r$  is defined as 
$$
\sB_p[\bc,r] \triangleq \{\bx\in\real^n\mid \norm{\bx-\bc}_p \leq r\}.
$$
For example, $\sB_2[\bzero,1]$ represents  the \textit{unit closed ball} w.r.t. to  the $\ell_2$ norm.
To simplify notation, we omit the subscript 2 for  $\ell_2$ norms and  $\bzero$ for balls centered at zero, e.g., $ \sB[1] \triangleq \sB_2[\bzero,1]  $.
As a special case, the notation $\sB_0[k] \triangleq\sB_0[\bzero, k]$ denotes the set of \textit{$k$-sparse vectors}, i.e., containing vectors that have only $k$ (or less) nonzero elements.
More generally, let $\norm{\cdot}$ be any norm, the induced open and closed balls are denoted as 
$$
\sB_{\norm{\cdot}} (\bc, r)
\qquad \text{and}\qquad 
\sB_{\norm{\cdot}} [\bc, r].
$$
\end{definition}

For a matrix $\bX\in\real^{n\times p}$, we  define the (matrix) Frobenius norm as follows.
\begin{definition}[Matrix Frobenius norm\index{Frobenius norm}]\label{definition:frobernius-in-svd}
The \textit{Frobenius norm} of a matrix $\bX\in \real^{n\times p}$ is defined as 
$$
\normf{\bX} = \sqrt{\sum_{i=1,j=1}^{n,p} (a_{ij})^2}=\sqrt{\trace(\bX\bX^\top)}=\sqrt{\trace(\bX^\top\bX)} = \sqrt{\sigma_1^2+\sigma_2^2+\ldots+\sigma_r^2}, 
$$
where $\sigma_1, \sigma_2, \ldots, \sigma_r$ are nonzero singular values of $\bX$, and  $\trace(\bX^\top\bX)$ denotes the trace of $\bX^\top\bX$, i.e., sum of diagonal elements of the matrix.
\end{definition}

The spectral norm is defined as follows.
\begin{definition}[Matrix spectral norm]\label{definition:spectral_norm}
The \textit{spectral norm} of a matrix $\bX\in \real^{n\times p}$ is defined as 
$$
\normtwo{\bX} = \mathop{\max}_{\bbeta\neq\bzero} \frac{\normtwo{\bX\bbeta}}{\normtwo{\bbeta}}  =\mathop{\max}_{\bu\in \real^p: \norm{\bu}_2=1}  \normtwo{\bX\bu} ,
$$
which is also the maximum singular value of $\bX$, i.e., $\normtwo{\bX} = \sigma_{\max}(\bX)$.
\end{definition}

We note that the Frobenius norm serves as the matrix counterpart of the6 vector $\ell_2$ norm.
For simplicity, we do not give the full subscript of the norm for the vector $\ell_2$ norm or Frobenius norm when it is clear from the context which one we are referring to: $\norm{\bX}=\normf{\bX}$ and $\norm{\bx}=\normtwo{\bx}$.
%In other words, when the underlying norm on $\real^n$ is $\norm{\cdot}_2$ or the , the subscript 2 will be frequently omitted; 
However, for the spectral norm, the subscript $\normtwo{\bX}$ should \textbf{not} be omitted.

The vector space $\real^n$, together with a given norm $\norm{\cdot}$, is called a \textit{normed vector space}.
On the other hand, one way to define norms for matrices is by viewing a matrix $\bX\in\real^{n\times p}$ as a vector in $\real^{np}$, e.g., using the vectorization of the matrix.
What distinguishes a matrix norm is a property called \textit{submultiplicativity}: $\norm{\bX\bY}\leq \norm{\bX}\norm{\bY}$ if $\norm{\cdot}$ is a submultiplicative matrix norm (see discussions below). 

In some texts, a matrix norm that is not \textit{submultiplicative} is termed as a \textit{vector norm on matrices} or a \textit{generalized matrix norm}.
The submultiplicativity of a matrix norm is important for the analysis of square matrices, although the definition of a matrix norm applies to both square and rectangular matrices.
For a submultiplicative matrix norm $\norm{\cdot}$ that satisfies $\norm{\bX\bY}\leq \norm{\bX}\norm{\bY}$, considering $\bX\in\real^{n\times n}$, it follows that 
\begin{equation}\label{equation:power_subm}
	\norm{\bX^2} \leq \norm{\bX}^2
	\quad\implies\quad
	\norm{\bX^k} \leq \norm{\bX}^k, \ \forall\, k\in\{1,2,\ldots,\}.
\end{equation}
Therefore, if the matrix is idempotent, i.e., $\bX^2=\bX$, we have $\norm{\bX}\geq 1$, which also indicates 
\begin{equation}\label{equation:power_subm2}
	\norm{\bI}\geq 1,
	\quad\text{if}\quad
	\norm{\cdot} \text{ is submultiplicative}.
\end{equation}
On the other hand, if $\bX$ is nonsingular, we have the inequality for submultiplicative norms: 
$$
1\leq \norm{\bI}=\norm{\bX\bX^{-1}}\leq \norm{\bX}\norm{\bX^{-1}}.
$$  
That is, a submultiplicative norm has $\norm{\bI}\geq 1$ and is \textit{normalized} if and only if $\norm{\bI}=1$.

\index{Submultiplicativity}
\index{Orthogonally invariance}
\begin{proposition}[Submultiplicativity and orthogonally invariance of Frobenius/spectral]\label{propo:submul_ortho_matnorm}
The Frobnenius and spectral norms are submultiplicative. That is, 
$$
\norm{\bX\bY}_F\leq \norm{\bX}_F\norm{\bY}_F
\qquad \text{and}\qquad 
\norm{\bX\bY}_2\leq \norm{\bX}_2\norm{\bY}_2.
$$
The two norms are also  and orthogonally invariant. That is, let $\bU\in \real^{n\times n}$ and $\bV\in \real^{p\times p}$ be orthogonal matrices, and let $\bX\in\real^{n\times p}$. Then,
$$
\norm{\bX }_F =  \norm{\bU\bX\bV }_F
\qquad \text{and}\qquad 
\norm{\bX }_2 =  \norm{\bU\bX\bV }_2.
$$
\end{proposition}

\index{Cauchy-Schwarz inequality}\index{Inequalities}
\section{Famous Inequalities}\label{section:inequalities}
In this section, we introduce some famous inequalities that will be often used. 
When considering random matrices, the \textit{Cauchy–Schwarz inequality}  is considered one of the most important and widely used inequalities in mathematics.
\begin{lemma}[Cauchy-Schwarz inequality]\label{lemma:cs_ineq_exp}
For any random $n\times p$ matrices $\rmX$ and $\rmY$ (see Section~\ref{section:stat_prob} for more details), we have 
$$
\Exp\left[\norm{\rmX^\top \rmY}\right] \leq \Exp\left[\norm{\rmX}^2\right]^{1/2} \Exp\left[\norm{\rmY}^2\right]^{1/2},
$$
where the inner product is defined as $\langle\rmX,\rmY\rangle = \Exp\left[\norm{\rmX^\top \rmY}\right]$.
\end{lemma}

\index{Schwarz matrix inequality}
The result can be applied to non-random matrices and vectors.
\begin{lemma}[Cauchy-Schwarz matrix (vector) inequality]
For any  $n\times p$ matrices $\bX$ and $\bY$, we have 
$$
\norm{\bX^\top \bY} \leq \norm{\bX} \cdot \norm{\bY}.
$$
This is a special form of the Cauchy-Schwarz inequality, where the inner product is defined as $\langle\bX,\bY\rangle = \norm{\bX^\top \bY}$.

Similarly, for any vectors $\bu, \bv$, we have 
\begin{equation}\label{equation:vector_form_cauchyschwarz}
\abs{\bu^\top \bv} \leq \Vert\bu\Vert \cdot  \Vert\bv\Vert.
\end{equation}
In the two-dimensional case, it becomes
$$
(ac+bd)^2 \leq (a^2 +b^2)(c^2+d^2).
$$
\end{lemma}
The vector form of the Cauchy-Schwarz inequality plays a crucial  role in various branches of modern mathematics, including Hilbert space theory and numerical analysis \citep{wu2009various}. 
Here, we present the proof for the vector form of the Cauchy-Schwarz inequality for simplicity.  To see this, given two vectors $\bu,\bv\in \real^n$, we have 
$$
\begin{aligned}
0\leq \sum_{i=1}^{n}\sum_{j=1}^{n} (u_i v_j - u_j v_i)^2 &= 
\sum_{i=1}^{n}\sum_{j=1}^{n} u_i^2v_j^2 + \sum_{i=1}^{n}\sum_{j=1}^{n} v_i^2 u_j^2 - 2\sum_{i=1}^{n}\sum_{j=1}^{n} u_iu_j v_iv_j\\
&=\left(\sum_{i=1}^{n} u_i^2\right) \left(\sum_{j=1}^{n} v_j^2\right) +
\left(\sum_{i=1}^{n} v_i^2\right) \left(\sum_{j=1}^{n} u_j^2\right) - 
2\left(\sum_{i=1}^{n} u_iv_i \right)^2\\
&=2 \norm{\bu}^2 \cdot \norm{\bv}^2 -2 \norm{\bu^\top\bv}^2,
\end{aligned}
$$
from which the result follows.
The equality holds if and only if $\bu = k\bv$ for some constant $k\in \real$, i.e., $\bu$ and $\bv$ are linearly dependent.

\index{Markov's inequality}
\begin{lemma}[Markov's inequality]\label{lemma:markov-inequality}
Let $\rx$ be a nonnegative random variable. Then, given any $\epsilon >0$, we have 
$$
\prob[\rx \geq \epsilon] \leq \frac{\Exp[\rx]}{\epsilon}.
$$
\end{lemma}
\begin{proof}[of Lemma~\ref{lemma:markov-inequality}]
We notice the trick that $0\leq \epsilon \indicator\{\rx \geq \epsilon\} \leq \rx$ since $\rx$ is nonnegative. This implies $\Exp[\epsilon \indicator\{\rx\geq \epsilon\}] \leq \Exp[\rx]$. We also have 
$$
\Exp\left[\epsilon \indicator\{\rx\geq \epsilon\}\right] = \epsilon\Exp\left[ \indicator\{\rx\geq \epsilon\}\right] = \epsilon \left(1\cdot \prob[\rx\geq \epsilon] + 0\cdot \prob[\rx< \epsilon] \right) = \epsilon \cdot \prob[\rx\geq \epsilon] \leq \Exp[\rx].
$$
This completes the proof.
\end{proof}

\index{Chebyshev's inequality}
\begin{lemma}[Chebyshev's inequality]
Let $\rx$ be a random variable with finite mean $\Exp[\rx]< \infty$. Then, given any $\epsilon >0$, we have 
$$
\prob[|\rx - \Exp[\rx]| \geq \epsilon ] \leq \frac{\Var[\rx]}{\epsilon^2}.
$$
\end{lemma}
Chebyshev's inequality can be easily verified by defining $\ry = (\rx - \Exp[\rx])^2$ (which is nonnegative) and applying Markov's inequality to $\ry$.

\section{Differentiability and Differential Calculus}
\index{Continuously differentiability}
\index{Second-order partial derivative}

Differentiability and differential calculus form the backbone of mathematical analysis, particularly in the study of functions defined over multidimensional spaces. This section delves into the fundamental concepts that enable us to understand how functions change with respect to their variables, providing a rigorous technique for analyzing and optimizing complex systems.

At the heart of this discussion is the concept of the directional derivative, which measures the rate of change of a function $f$ at a point $\bx$ in the direction of a vector $\bd$.
\begin{definition}[Directional derivative, partial derivative]\label{definition:partial_deri}
Given a function $f$ defined over a set $\sS\subseteq \real^n$ and a nonzero vector $\bd\in\real^n$. Then the \textit{directional derivative} of $f$ at $\bx$ w.r.t. the direction $\bd$ is given by, if the limit exists, 
$$
\mathop{\lim}_{t\rightarrow 0^+}
\frac{f(\bx+t\bd) - f(\bx)}{t}.
$$
And it is denoted by $f^\prime(\bx; \bd)$ or $D_{\bd}f(\bx)$. 
The directional derivative is sometimes called the \textit{G\^ateaux derivative}.

For any $i\in\{1,2,\ldots,n\}$, the directional derivative at $\bx$ w.r.t. the direction of the $i$-th standard basis $\be_i$ is called the $i$-th \textit{partial derivative} and is denoted by $\frac{\partial f}{\partial x_i} (\bx)$, $D_{\be_i}f(\bx)$, or $\partial_i f(\bx)$.
\end{definition}

If all the partial derivatives of a function $f$ exist at a point $\bx\in\real^n$, then the \textit{gradient} of $f$ at $\bx$ is defined as the column vector containing all the partial derivatives:
$$
\nabla f(\bx)\triangleq
\begin{bmatrixfoot}
\frac{\partial f}{\partial x_1} (\bx)\\
\frac{\partial f}{\partial x_2} (\bx)\\
\vdots \\
\frac{\partial f}{\partial x_n} (\bx)	
\end{bmatrixfoot}
\in \real^n.
$$
A function $f$ defined over an open set $\sS\subseteq \real^n$ is called \textit{continuously differentiable} over $\sS$ if all the partial derivatives exist and are continuous on $\sS$.
In the setting of continuously differentiability, the directional derivative and gradient have the following relationship:
\begin{equation}
f^\prime(\bx; \bd) = \nabla f(\bx)^\top \bd, \gap \text{for all }\bx\in\sS \text{ and }\bd\in\real^n.
\end{equation} 
And in the setting of continuously differentiability, we also have
\begin{equation}
\mathop{\lim}_{\bd\rightarrow \bzero}
\frac{f(\bx+\bd) - f(\bx) - \nabla f(\bx)^\top \bd}{\norm{\bd}} = 0\gap 
\text{for all }\bx\in\sS,
\end{equation}
or 
\begin{equation}
f(\by) = f(\bx)+\nabla f(\bx)^\top (\by-\bx) + o(\norm{\by-\bx}),
\end{equation}
where the \textit{small-oh} function $o(\cdot): \real_+\rightarrow \real$ is a one-dimensional function satisfying $\frac{o(t)}{t}\rightarrow 0$ as $t\rightarrow 0^+$.~\footnote{Note that we also use the standard \textit{big-Oh} notation to describe the asymptotic behavior of functions.
Specifically, the notation $g(\bd) = \mathcalO(\normtwo{\bd}^p)$ means that there are positive numbers
$C_1$ and $\delta$ such that $\abs{g(\bd)} \leq  C_1 \normtwo{\bd}^p$ for all $\normtwo{\bd}\leq\delta$. In practice it is
often equivalent to $\abs{g(\bd)} \approx C_2\normtwo{\bd}^p$ for  sufficiently small $\bd$, where $C_2$ is  another positive constant.
The \textit{soft-Oh} notation is employed to hide poly-logarithmic factors i.e., $f = \widetilde{\mathcalO}(g)$ will
imply $f = \mathcalO(g \log^c(g))$ for some absolute constant $c$.}

The partial derivative $\frac{\partial f}{\partial x_i} (\bx)$ is also a real-valued function of $\bx\in\sS$ that can be partially differentiated. The $j$-th partial derivative of $\frac{\partial f}{\partial x_i} (\bx)$ is defined as 
$$
\frac{\partial^2 f}{\partial x_j\partial x_i} (\bx)=
\frac{\partial \left(\frac{\partial f}{\partial x_i} (\bx)\right)}{\partial x_j} (\bx).
$$
This is called the ($j,i$)-th \textit{second-order partial derivative} of function $f$.
A function $f$ defined over an open set $\sS\subseteq$ is called \textit{twice continuously differentiable} over $\sS$ if all the second-order partial derivatives exist and are continuous over $\sS$. In the setting of twice continuously differentiability, the second-order partial derivative are symmetric:
$$
\frac{\partial^2 f}{\partial x_j\partial x_i} (\bx)=
\frac{\partial^2 f}{\partial x_i\partial x_j} (\bx).
$$
The \textit{Hessian} of the function $f$ at a point $\bx\in\sS$ is defined as the symmetric $n\times n$ matrix 
$$
\nabla^2f(\bx)\triangleq
\begin{bmatrixfoot}
\frac{\partial^2 f}{\partial x_1^2} (\bx) & 
\frac{\partial^2 f}{\partial x_1\partial x_2} (\bx) & \ldots &
\frac{\partial^2 f}{\partial x_1\partial x_n} (\bx)\\
\frac{\partial^2 f}{\partial x_2\partial x_1} (\bx) & 
\frac{\partial^2 f}{\partial x_2\partial x_2} (\bx) & \ldots &
\frac{\partial^2 f}{\partial x_2\partial x_n} (\bx)\\
\vdots & 
\vdots & \ddots &
\vdots\\
\frac{\partial^2 f}{\partial x_n\partial x_1} (\bx) & 
\frac{\partial^2 f}{\partial x_n\partial x_2} (\bx) & \ldots &
\frac{\partial^2 f}{\partial x_n^2} (\bx)
\end{bmatrixfoot}.
$$

We then provide a simple proof of Taylor's expansion   for one-dimensional functions. 
\index{Taylor's expansion}
\begin{theoremHigh}[Taylor's expansion with Lagrange remainder]\label{theorem:taylo_exp}
Let $f(x): \real\rightarrow \real$ be $k$-times continuously differentiable on the closed interval $\sI$ with endpoints $x$ and $y$, for some $k\geq 0$. If $f^{(k+1)}$ exists on the interval $\sI$, then there exists a $\xi \in (x,y)$ such that 
$$
\begin{aligned}
f(x)& = f(y) + f^\prime(y)(x-y) +\ldots + \frac{f^{(k)}(y)}{k!}(x-y)^k
+ \frac{f^{(k+1)}(\xi)}{(k+1)!}(x-y)^{k+1}\\
&=\sum_{i=0}^{k} \frac{f^{(i)}(y)}{i!} (x-y)^i + \frac{f^{(k+1)}(\xi)}{(k+1)!}(x-y)^{k+1}.
\end{aligned}
$$ 
Taylor's expansion can be extended to a function of vector $f(\bx):\real^n\rightarrow \real$ or a function of matrix $f(\bX): \real^{m\times n}\rightarrow \real$.
\end{theoremHigh}
Taylor's expansion, or also known as  \textit{Taylor's series}, approximates the function $f(x)$ around a value  $y$ using a polynomial in a single variable $x$. To understand the origin of this series, we recall from the elementary calculus course that the approximated function  of $\cos (\theta)$ around $\theta=0$ is given by 
$
\cos (\theta) \approx 1-\frac{\theta^2}{2}.
$
This means that $\cos(\theta)$ can be approximated by a second-degree polynomial. 
If we want to approximate $\cos(\theta)$ more generally with a second-degree polynomial $ f(\theta) = c_1+c_2 \theta+ c_3 \theta^2$, an intuitive approach is to match the function and its derivatives at $\theta=0$. That is,
$$\left\{
\begin{aligned}
\cos(0) &= f(0); \\
\cos^\prime(0) &= f^\prime(0);\\
\cos^{\prime\prime}(0) &= f^{\prime\prime}(0);\\
\end{aligned}
\right.
\quad\implies\quad 
\left\{
\begin{aligned}
1 &= c_1; \\
-\sin(0) &=0= c_2;\\
-\cos(0) &=-1= 2c_3.\\
\end{aligned}
\right.
$$
Solving these equations yields $f(\theta) = c_1+c_2 \theta+ c_3 \theta^2 = 1-\frac{\theta^2}{2}$, which matches our initial approximation $\cos (\theta) \approx 1-\frac{\theta^2}{2}$ around  $\theta=0$. 

For high-dimensional functions, we have the following approximation results.
\begin{theoremHigh}[Mean value theorem]\label{theorem:mean_approx}
Let $f(\bx):\sS\rightarrow \real$ be a  continuously differentiable function over an open set $\sS\subseteq\real^n$, and given two points $\bx, \by\in\sS$. Then, there exists a point $\bxi\in[\bx,\by]$ such that 
$$
f(\by) = f(\bx)+ \nabla f(\bxi)^\top (\by-\bx).
$$

\end{theoremHigh}

\begin{theoremHigh}[Linear approximation theorem]\label{theorem:linear_approx}
Let $f(\bx):\sS\rightarrow \real$ be a twice continuously differentiable function over an open set $\sS\subseteq\real^n$, and let $\bx, \by\in\sS$. Then, there exists a point $\bxi\in[\bx,\by]$ such that 
$$
f(\by) = f(\bx)+ \nabla f(\bx)^\top (\by-\bx) + \frac{1}{2} (\by-\bx)^\top \nabla^2 f(\bxi) (\by-\bx),
$$ 
or 
$$
f(\by) = f(\bx)+\nabla f(\bx)^\top (\by-\bx) + o(\normtwo{\by-\bx}),
$$
or
$$
f(\by) = f(\bx)+\nabla f(\bx)^\top (\by-\bx) + \mathcalO(\normtwo{\by-\bx}^2).
$$
\end{theoremHigh}
This theorem suggests that the error in the linear approximation is of the order of the square of the distance between  $\bx$ and $\by$.
\begin{theoremHigh}[Quadratic approximation theorem]\label{theorem:quad_app_theo}
Let $f(\bx):\sS\rightarrow \real$ be a twice continuously differentiable function over an open set $\sS\subseteq\real^n$, and let $\bx, \by\in\sS$. Then it follows that 
$$
f(\by) = f(\bx)+ \nabla f(\bx)^\top (\by-\bx) + \frac{1}{2} (\by-\bx)^\top \nabla^2 f(\bx) (\by-\bx)
+
o(\normtwo{\by-\bx}^2),
$$
or 
$$
f(\by) = f(\bx)+ \nabla f(\bx)^\top (\by-\bx) + \frac{1}{2} (\by-\bx)^\top \nabla^2 f(\bx) (\by-\bx)
+
\mathcalO(\normtwo{\by-\bx}^3).
$$
\end{theoremHigh}
This theorem indicates that the error in the quadratic approximation is of the order of the cube of the distance between $\bx$ and $\by$, making it a more accurate approximation when $\by$ is close to $\bx$.

\section{Statistics and Common Probability Distributions}\label{section:stat_prob}

A random variable is a variable that assumes different values randomly that models uncertain outcomes or events. We denote the random variable itself with a lowercase letter in \textit{normal fonts}, and its possible values with lowercase letters in \textit{italic fonts}.  
For instance, $y_1$ and $y_2$ are  possible values of the random variable $\ry$. 
In the case of vector-valued variables, we represent the random variable as $\rvy$ and one of its realizations as $\by$. 
Similarly, we denote the random variable as $\rmY$ and one of its values as $\bY$ when we are working with matrix-valued variables. 
However, a random variable merely describes possible states and needs to be accompanied by a probability distribution specifying the likelihood of each state.

\paragrapharrow{Probability and statistics.}
In a general sense, one can describe statistics as the mathematical discipline whose purpose is to use empirical data generated by a random phenomenon, in order to make inferences about certain deterministic characteristics of the phenomenon, while simultaneously quantifying the uncertainty inherent in these inferences.

Let's take a step back and examine the key components of this definition. What exactly  is a \textit{random phenomenon}? 
A random phenomenon can be thought of as a system or process whose outcome, denoted $\rx$, is uncertain.  
This means that even if we have complete knowledge of all aspects of the system, we still cannot predict its outcome with certainty.
In mathematical terms, such phenomena are modeled using probability theory: the outcome $\rx$ is represented as a random variable, and the model describing the phenomenon is its \textit{probability distribution function}, also known as the \textit{cumulative distribution function (CDF)}, defined as:
$$
F(x) \equiv \Pr[\rx \leq x].
$$ 
Now, suppose there is a characteristic $\btheta$ of the phenomenon that affects the probabilities associated with the outcomes of $\rx$. Such a characteristic is referred to as a \textit{parameter}. Since the probability of the event $\{\rx \leq x\}$ depends on $\btheta$, the function $F(x)$ must also depend on $\btheta$. Therefore, we write it as:
$$
F(x; \btheta) = {\Pr}_{\btheta}[\rx \leq x].
$$

In probability theory, if we know both the functional form of $F(x; \btheta)$ and the true value of $\btheta$, we can then calculate the probability $\Pr_{\btheta}[\rx \leq x] = F(x; \btheta)$ for any possible outcome $x$. 
However, in statistics, we deal with the inverse problem: suppose that we know the precise functional form of $F(x; \btheta)$, but do not know which is the true $\btheta$. If we have an outcome $x$ (a realization of $\rx$), the central question becomes: is it possible to say something useful about $\btheta$? It seems that we should be able to do so. Since $\btheta$ influences what outcomes are most probable, then knowing an outcome should give us information on which $\btheta$ are plausible. The topic of statistics will be how exactly to make this connection rigorous and show how to exploit it in order to (a) make the best possible use of our data $x$ to better inform ourselves about $\btheta$ and (b) understand how certain we can be about our inferences on $\btheta$ for the given data $x$. In summary, the discussion of statistics includes:
\begin{enumerate}
\item There is a distribution $F(x; \btheta)$ depending on an unknown $\btheta \in \real^p$.
\item We observe the realization of $n$ independent identically distributed random variables $\rx_1, \rx_2,\ldots, \rx_n$ that follow this distribution.
\item We wish to use our $n$ observations (the realizations of $\rx_1, \rx_2, \ldots, \rx_n$) in order to make statements about the true value of $\btheta$ and to quantify the uncertainty associated with those statements.
\end{enumerate}

\paragrapharrow{Discrete  random variables.}
Random variables can be either discrete or continuous. A discrete random variable has a finite or countably infinite number of states. These states need not be integers; they can also be named states without numerical values. Conversely, a continuous random variable is associated with real values.

A probability distribution describes the likelihood of a random variable or set of random variables. 
The characterization of probability distributions varies depending on whether the variables are discrete or continuous.

For discrete variables, we employ a \textit{probability mass function (p.m.f., PMF)}. Probability mass functions are denoted with a capital $\Pr$.
The probability mass function maps  a state of a random variable to
the probability of that random variable taking on that state. The notation $\Pr(\ry = y)$ (or $f_{\ry}(y)$, $\Pr_{\ry}(y)$) represents the probability that $\ry$ equals $y$, with a probability of 1 indicating certainty and a probability of 0 indicating impossibility. 
Alternatively, we define a variable first and use the $\sim$ notation to specify its distribution later: $\ry \sim \Pr(\ry)$.

Probability mass functions can operate on multiple variables simultaneously, constituting a \textit{joint probability mass distribution functions} or \textit{joint frequency functions}. 
For example, $\Pr(\rx = x, \ry = y)$ denotes the probability of $x = x$ and $y = y$ simultaneously, and we may also use the shorthand $\Pr(x, y)$ or $\Pr_{\rx, \ry}(x,y)$. Moreover, if the PMF depends on some known parameters $\btheta$, then it can be denoted by $\Pr(x, y \mid \btheta)$  (or $\Pr_{\btheta}(x, y)$, $f(x, y; \btheta)$ ) for brevity.

In many cases, our focus lies in determining the probability of an event, given the occurrence of another event. This is referred to as a \textit{conditional probability}. 
The conditional probability that $\rx = x$ given $\ry = y$ is denoted by  $\Pr(\rx = x \mid  \ry = y)$. 
This can be calculated using the formula
$$
\Pr(\rx = x \mid  \ry = y) = \frac{\Pr(\rx = x, \ry = y)}{\Pr(\ry = y)}.
$$
This formula serves as the cornerstone in \textit{Bayes' theorem} (see Theorem~\ref{theorem:baye_theo_mle}).

On the contrary, there are instances when the probability distribution across a set of variables is known, and the interest lies in determining the probability distribution over a specific subset of them.  
The probability distribution over the subset is referred to as the \textit{marginal probability mass distribution}.
For example, suppose we have discrete random variables $\rx$ and $\ry$, and we know $\Pr(\rx, \ry)$. We can find $\Pr(\rx)$ using summation:
$$
\Pr(\rx=x) = \sum_{y} \Pr(\rx=x, \ry=y).
$$

\paragrapharrow{Continuous random variables.}
When dealing with continuous random variables, we represent  probability distributions using a \textit{probability density function (p.d.f., PDF)} instead of  a probability mass function. For a function $p$ to qualify as a probability density function, it must adhere to the following properties:
\begin{itemize}
	\item The domain of $p$ must be the set of all possible states of $\ry$;
	\item We do not require $p(y) \leq  1$ as that in the PMF. However, it must satisfies that $\forall\, y\in\ry, p(y)\geq 0$.
	\item Integrates to 1: $\int p(y)dy = 1$.
\end{itemize}
A probability density function $p(y)$ (or denoted as $f_{\ry}(y)$, $p_{\ry}(y)$) does not provide the probability of a specific state directly. 
Instead, the probability of falling within an infinitesimal region with volume $\delta y$ is given by $p(y)\delta y$.
Moreover, if the probability density function depends on some known parameters $\btheta$, it can be denoted by $p(x\mid \btheta)$, $f_{\rx}(x; \btheta)$ or $f(x; \btheta)$ for brevity.

%\footnote{The notation $p(y\mid \ldots)$ will be used to mean the probability of $\ry = y$ in the case that $\mathcal{Y}$ is countable, and to represent the density $p(\ry \in [y, y + dy])/dy$ in the case that $\mathcal{Y}$ is continuous.}

\index{Variance}
\paragrapharrow{Distribution function.}
On the other hand, a probability distribution can be characterized by the cumulative distribution function (c.d.f., CDF).
Let $\rx$ be a random variable with the  distribution function $F(x)=\Pr[\rx\leq x]$, where $F(x)$ is nondecreasing and right continuous and satisfies
$$
0 \leq F(x) \leq 1, \qquad F(-\infty) = 0, \qquad F(\infty) = 1.
$$
The expected value $\mu$ and variance $\omega^2$ of $\rx$ in the continuous case are defined as
$$
\mu = \Exp[\rx] = \int_{-\infty}^{\infty} x dF(x), \qquad \omega^2=\Var[\rx] = \Exp[(\rx - \mu)^2] = \int_{-\infty}^{\infty} (x - \mu)^2 dF(x).
$$
Similarly, in the discrete case, they are defined as 
$$
\mu = \Exp[\rx] = \sum_{x} x \Pr(x), \qquad \omega^2 =\Var[\rx]= \Exp[(\rx - \mu)^2] = \sum_{x} (x - \mu)^2 \Pr(x).
$$
It is easy to see that 
\begin{equation}
\Var[\rx] = \Exp[\rx^2] - (\Exp[\rx])^2.
\end{equation}
\begin{exercise}[Uniform distribution]\label{exercise:uniform_dist}
Let $\rx\sim \uniformdist(x\mid a,b)$ be a uniform distributed variable such that $f_{\rx}(x) = \frac{1}{b-a}$ if $a\leq x\leq b$ and $f_{\rx}(x) =0$ otherwise. Show that
$$
\Exp[\rx] = \frac{a+b}{2}
\qquad \text{and}\qquad 
\Var[\rx] =\frac{(b-a)^2}{12}.
$$
\end{exercise}

Let $\rvx = [\rx_1, \rx_2 \ldots, \rx_n]^\top$ be a vector of random variables. The \textit{joint probability  distribution function} (or simply probability distribution function) of the random vector---denoted as $F_{\rvx}(\bx) $, $F(x_1, x_2, \ldots, x_n)$, or $F(\bx)$---is defined as 
$$
 F_{\rvx}(x_1, x_2, \ldots, x_n) = \Pr[\rx_1\leq x_1, \rx_2, \leq x_2, \ldots, \rx_n\leq x_n].
$$
It holds that 
\begin{itemize}
\item When all the variables are discrete, as discussed above, the joint probability mass function or the joint frequency function can be characterized as 
$$
f_{\rvx}(x_1, x_2, \ldots, x_n) = \Pr[\rx_1=x_1, \rx_2=x_2, \ldots, \rx_n=x_n].
$$
\item On the contrary, if the variables are continuous, the \textit{joint probability density function} is a function $f_{\rvx} :\real^n\rightarrow [0, \infty)$ (or simply $f(\bx)$, $p_{\rvx}(\bx)$, $p(\bx)$) such that 
$$
F_{\rvx}(x_1, x_2, \ldots, x_n) = \int_{-\infty}^{x_1}\ldots \int_{-\infty}^{x_n} f_{\rvx}(y_1, y_2, \ldots, y_n)d y_1 \ldots dy_n.
$$
In this case, when $f_{\rvx}$ is continuous at $\bx=[x_1,x_2, \ldots,x_n]^\top$, we have 
$$
f_{\rvx}(x_1,x_2, \ldots,x_n) =\frac{\partial^n}{\partial x_1 \partial x_2\ldots \partial x_n} F_{\rvx}(x_1,x_2, \ldots,x_n).
$$
The variables are independent if and only if 
$$
\begin{aligned}
F_{\rvx}(x_1,x_2, \ldots,x_n) &= F_{\rx_1}(x_1) \cdot F_{\rx_2}(x_2)\cdot\ldots \cdot F_{\rx_n}(x_n);\\
\text{or} \quad f_{\rvx}(x_1,x_2, \ldots,x_n) &= f_{\rx_1}(x_1) \cdot f_{\rx_2}(x_2)\cdot\ldots \cdot f_{\rx_n}(x_n).
\end{aligned}
$$
The \textit{conditional or marginal probability density functions} for the continuous cases are similar to the discrete cases, except that integration is employed instead of summation. For example, $f_{\rx}(x) = \int f_{\rx,\ry}(x, y)dy$.
\end{itemize}

\index{Covariance}
\index{Correlation}
Specifically, the joint distribution of $\rx_i$ and $\rx_j$ is $F(x_i, x_j) \equiv \Pr[\rx_i\leq x_i, \rx_j\leq x_j]$. Then the covariance $\omega_{ij}$ between $\rx_i$ and $\rx_j$ is defined by
$$
\omega_{ij} \triangleq \Cov[\rx_i, \rx_j] \triangleq \Exp[(\rx_i - \mu_i)(\rx_j - \mu_j)] = \int_{x_i, x_j = -\infty}^{\infty} (x_i - \mu_i)(x_j - \mu_j) dF(x_i, x_j).
$$
Then $\omega_{ij} = \Exp[\rx_i \rx_j] - \mu_i \mu_j$, where $\mu_i = \Exp[\rx_i]$. The covariance matrix $\bOmega \in \real^{n \times n}$ of the random vector $\rvx$ is defined by
\begin{equation}
\Cov[\rvx] \triangleq \bOmega = \Exp[(\rvx - \bmu)(\rvx - \bmu)^\top] = \Exp(\rvx\rvx^\top) - \bmu \bmu^\top,
\end{equation}
where $\bmu = \Exp[\rvx] = [\mu_1, \mu_2, \ldots, \mu_n]^\top$. A few useful properties are provided below.

The variance of two random variables expresses the degree of linear dependency between the two. Moreover, the correlation between the two is defined as 
\begin{equation}
\Corr[\rx_1, \rx_2] \triangleq \frac{\Cov[\rx_1, \rx_2]}{\sqrt{\Var[\rx_1]\Var[\rx_2]}}.
\end{equation}
The variance and correlation of two random variable convey the equivalent dependence information. 
However, the correlation is invariant to changes of scale, and the range of the correlation is $[-1,1]$.
It also holds that 
\begin{equation}
\abs{\Corr[\rx_1, \rx_2]} \leq \sqrt{\Var[\rx_1]\Var[\rx_2]}.
\end{equation}

\begin{lemma}[Linear transformation]\label{lemma:lin_tran_prob}
Let $\rvy = \bA\rvx$, where $\bA \in \real^{m \times n}$ is a given matrix, and let $\rvx \in \real^n$ be a random vector with the mean vector $\Exp[\rvx] = \bmu$ and covariance matrix $\bOmega$. Then
$$
\Exp[\rvy] = \bA\bmu 
\qquad\text{and}\qquad  \Cov[\rvy] = \bA\bOmega\bA^\top.
$$
\end{lemma}
\begin{proof}[of Lemma~\ref{lemma:lin_tran_prob}]
The first property follows directly from the definition of expected value. The second is proved as
$$
\Cov[\bA\rvx] = \Exp[\bA(\rvx - \bmu)(\rvx - \bmu)^\top \bA^\top ]= \bA \Exp[(\rvx - \bmu)(\rvx - \bmu)^\top] \bA^\top = \bA\bOmega\bA^\top.  
$$
This completes the proof.
\end{proof}
In the special case when $\bA=\ba^\top$ is a row vector, $\rvy = \ba^\top\rvx$ is a linear functional of $\rvx$. Then, if $\Cov[\rvx] = \sigma^2 \bI$, $\Cov[\rvy] = \sigma^2 \ba^\top \ba$. 
\begin{lemma}[Quadratic transformation]\label{lemma:quad_tra_prob}
Let $\bA \in \real^{n \times n}$ be a symmetric matrix, and let $\rvx\in\real^n$ be a random vector with expected value $\bmu=[\mu_i]$ and covariance matrix $\bOmega=[\omega_{ij}]$. Then,
$$
\Exp[\rvx^\top\bA\rvx] = \bmu^\top \bA \bmu + \trace(\bA\bOmega),
$$
where $\trace(\bA\bOmega)$ denotes the trace of $\bA\bOmega$, i.e., sum of diagonal elements of the matrix.
\end{lemma}
\begin{proof}[of Lemma~\ref{lemma:quad_tra_prob}]
Since $\rvx^\top \bA \rvx = \sum_{i=1}^n \sum_{j=1}^n a_{ij} \rx_i \rx_j$, 
it follows that  $\Exp[\rvx^\top \bA \rvx] = \sum_{i=1}^n \sum_{j=1}^n a_{ij} \Exp[\rx_i \rx_j]$.
Substitute the expectations $ \Exp[\rx_i \rx_j] = \mu_i \mu_j + \omega_{ij} $:
$$
\Exp[\rvx^\top \bA \rvx] = \underbrace{\sum_{i=1}^n \sum_{j=1}^n a_{ij} \mu_i \mu_j}_{\bmu^\top \bA \bmu} + \underbrace{\sum_{i=1}^n \sum_{j=1}^n a_{ij} \omega_{ij}}_{\trace(\bA\bOmega)}.
$$
Thus, the lemma is proven.

Alternatively, for $\rvx^\top \bA\rvx$, we have
$$
\begin{aligned}
	\rvx^\top \bA\rvx &= (\bmu + \rvx - \bmu )^\top \bA(\bmu + \rvx - \bmu) \\ 
	&=\bmu^\top \bA\bmu +\bmu^\top \bA(\rvx-\bmu) + (\rvx-\bmu)^\top \bA\bmu + (\rvx-\bmu)^\top\bA(\rvx-\bmu).
\end{aligned}
$$
Thus, it follows that 
$$
\begin{aligned}
\Exp[\rvx^\top\bA\rvx] 
&= \Exp[\bmu^\top \bA\bmu +\bmu^\top \bA(\rvx-\bmu)  + (\rvx-\bmu)^\top \bA\bmu + (\rvx-\bmu)^\top\bA(\rvx-\bmu)] \\
&=\bmu^\top \bA\bmu + \Exp[(\rvx-\bmu)^\top\bA(\rvx-\bmu)] 
= \bmu^\top \bA\bmu + \Exp\left[\trace[(\rvx-\bmu)^\top\bA(\rvx-\bmu)]\right]  \\
&\stackrel{\dag}{=}\bmu^\top \bA\bmu + \Exp\left[\trace[\bA(\rvx-\bmu)(\rvx-\bmu)^\top]\right]
\stackrel{\ddag}{=}\bmu^\top \bA\bmu + \trace\left[\Exp[\bA(\rvx-\bmu)(\rvx-\bmu)^\top]\right] \\
&=\bmu^\top \bA\bmu + \trace\left[\bA\Exp[(\rvx-\bmu)(\rvx-\bmu)^\top]\right]
=\bmu^\top \bA\bmu + \trace\left[\bA\bOmega\right],
\end{aligned}
$$
where the equality ($\dag$) follows from the trace trick, and the equality ($\ddag$) follows from the linear property of traces.
For any matrix $\bA, \bB, \bC$, the trace trick is
\begin{equation}
	\trace(\bA\bB\bC) = \trace(\bB\bC\bA) = \trace(\bC\bA\bB), \nonumber
\end{equation}
if all $\bA\bB\bC$, $\bB\bC\bA$, and $\bC\bA\bB$ exist. 
This completes the proof.
\end{proof}

We will often engineer this  $\bmu=\bzero$ so that the key thing we care about is just $\trace(\bA\bOmega)$.
If we assume that $\bmu=\bzero$ and the variance is $\sigma^2\bSigma$, then $\Exp[\rvx^\top\bA\rvx ] = \sigma^2\trace(\bA\bSigma)$.
Note that $\bSigma$ is a known value. Then, we search the universe to find an $\bA$ such that $\trace(\bA\bSigma) = 1$, which would give us an unbiased estimate of $\sigma^2$; see Definition~\ref{defintion:biased_unbiased} for a reference.

\index{Trace trick}
\index{Quadratic expectation}

\subsection{Common Univariate Probability Distributions}
In the rest of this section,  we provide rigorous definitions for common probability distributions.
\index{Gaussian distribution}
\begin{definition}[Gaussian or normal distribution]\label{definition:gaussian_distribution}
A random variable $\rx$ is said to follow the \textit{Gaussian distribution} (a.k.a., the \textit{normal distribution}) with mean and variance parameters $\mu$ and $\sigma^2>0$, denoted by $\rx \sim \normal(\mu,\sigma^2)$ \footnote{Note if two random variables $\ra$ and $\rb$ have the same distribution, then we write $\ra \sim \rb$.}, if 
$$
f(x; \mu,\sigma^2)=\frac{1}{\sqrt{2\pi\sigma^2}} \exp \left\{-\frac{1}{2\sigma^2 }(x-\mu)^2 \right\}
=\sqrt{\frac{\tau}{2\pi}}\exp \left\{ -\frac{\tau}{2}(x-\mu)^2 \right\} .
$$
The mean and variance of $\rx \sim \normal( \mu,\sigma^2)$ are given by 
$$
\Exp[\rx] = \mu, \qquad \Var[\rx] =\sigma^2=\tau^{-1}, 
$$
where $\tau$ is also known as the \textit{precision} of the Gaussian distribution.
The cumulative distribution function (c.d.f., CDF) of Gaussian is 
$$
F(x; \mu, \sigma^2) = \Pr(\rx<x) = \frac{1}{\sqrt{2\pi\sigma^2}} \int_{-\infty}^x \exp\left\{-\frac{1}{2\sigma^2 }(z-\mu)^2 \right\}dz.
$$
Specifically, we denote $\Phi(y) = \int_{-\infty}^{y} \normal(u\mid 0,1)du= \frac{1}{\sqrt{2\pi}} \int_{-\infty}^{y} \exp(-\frac{u^2}{2}) du $ as the cumulative distribution function of $\normal(0,1)$, the standard normal distribution. 
Figure~\ref{fig:dists_gaussian} illustrates the impact of different parameters $\mu, \sigma^2$ for the Gaussian distribution.
\end{definition}
Suppose $\mathcalX=\{x_1, x_2, ..., x_N\}$ are drawn independent, identically distributed (i.i.d.) from a Gaussian distribution of $\normal(x\mid \mu, \sigma^2)$.
For analytical simplicity, we may rewrite the Gaussian probability density function as follows:
\begin{equation}\label{equation:uni_gaussian_likelihood}
\begin{aligned}
p(\mathcal{X} \mid \smu, \ssigma^2) &= \prod^N_{i=1} \mathcal{N} (x_i\mid\smu, \ssigma^2) \\
&= (2\pi)^{-N/2}  (\ssigma^2)^{-N/2} \exp\left\{-\frac{1}{2 \ssigma^2}  \left[ N(\overline{x} - \smu)^2 + N \sum_{n=1}^N(x_n - \overline{x})^2   \right] \right\}  \\
&= (2\pi)^{-N/2}  (\ssigma^2)^{-N/2} \exp\left\{-\frac{1}{2 \ssigma^2}  \left[  N(\overline{x} - \smu)^2 +  N S_{\overline{x}} \right] \right\},
\end{aligned}
\end{equation}
where $S_{\overline{x}}\triangleq \sum_{n=1}^N(x_n - \overline{x})^2$ and $\overline{x} \triangleq \frac{1}{N} \sum_{i=1}^{N}x_i$.

\begin{SCfigure}%[H]
\centering
\includegraphics[width=0.5\textwidth]{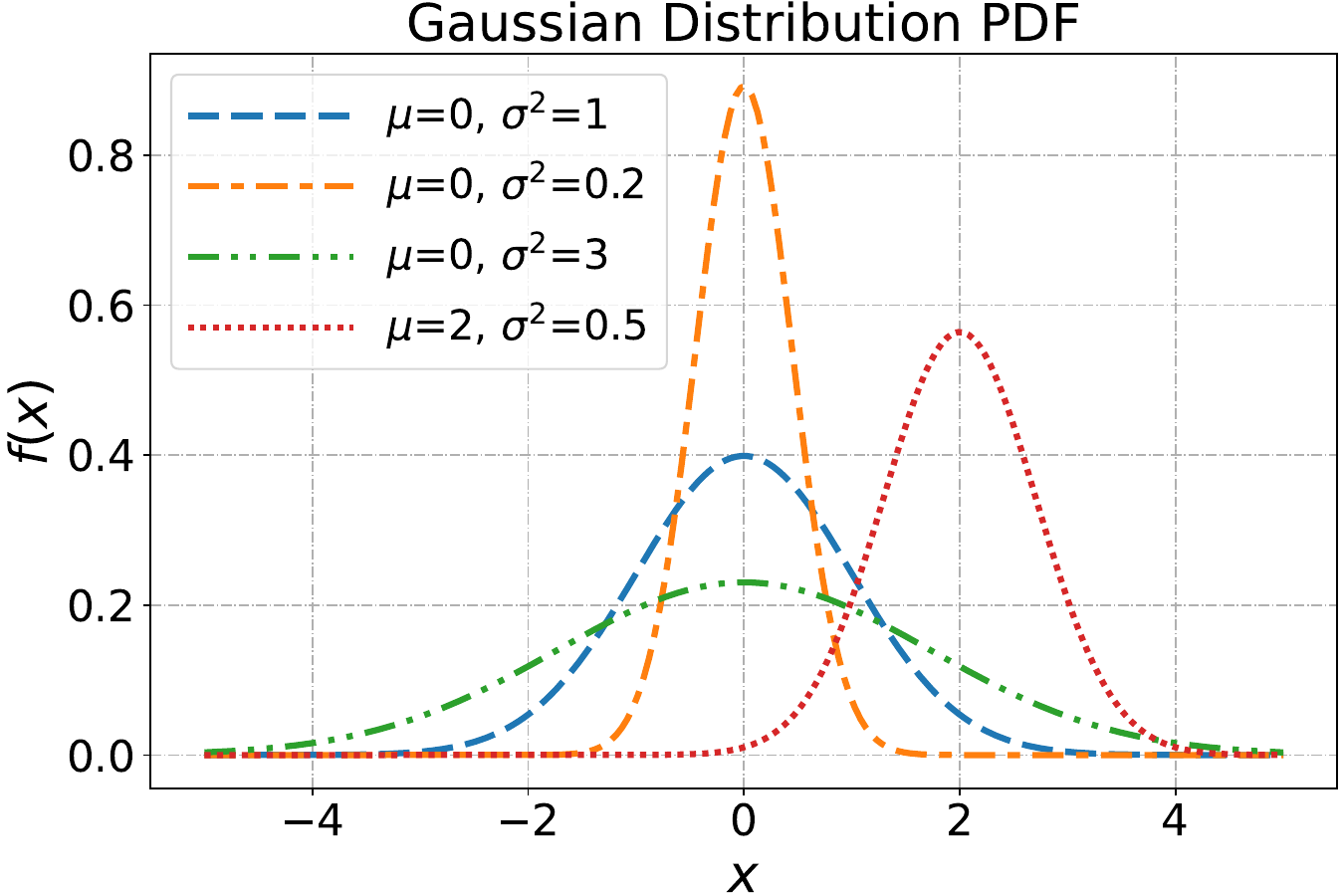}
\caption{Gaussian probability density functions for different values of the mean and variance parameters $\mu$ and $\sigma^2$.}
\label{fig:dists_gaussian}
\end{SCfigure}

While the product of two Gaussian variables remains an open problem, the sum of Gaussian variables can be characterized by a new Gaussian distribution.
\begin{remark}[Sum of Gaussians]
Let $\rx$ and $\ry$ be two Gaussian distributed variables with means $\mu_x, \mu_y$ and variance $\sigma_x^2, \sigma_y^2$, respectively.
\begin{itemize}
\item When there is no correlation between the two variables, then it follows that
$$
\rx+\ry \sim \normal(\mu_x+\mu_y, \sigma_x^2+\sigma_y^2).
$$
\item When there exists a correlation of $\rho$ between the two variables, then  it follows that
$$
\rx+\ry \sim \normal(\mu_x+\mu_y, \sigma_x^2+\sigma_y^2+2\rho\sigma_x\sigma_y).
$$
\end{itemize}
\end{remark}

Gaussian distributions have strong concentration properties. The  following tail bound on a Gaussian random variable is an important result.
\begin{exercise}[Chernoff bound for centered Gaussian]
For $\rx \sim \normal(0, \sigma^2)$, show that
$$
p(\abs{\rx} \geq t) \leq 2 e^{-\frac{t^2}{2\sigma^2}}.
$$
\textit{Hint: Use Chernoff bound.}
\end{exercise}

The \textit{Laplace} distribution, also known as  the \textit{double exponential} distribution, is named after \textit{Pierre-Simon Laplace} (1749--1827), who obtained the distribution in 1774 \citep{kotz2001laplace, hardle2007applied}.
This distribution finds applications in modeling heavy-tailed data due to its tails being heavier than those of the normal distribution, and it is used extensively in sparse-favoring models since it expresses a high peak with heavy tails (same as the $\ell_1$ regularization term in non-probabilistic or non-Bayesian optimization methods). 
In Bayesian modeling, when there is a prior belief that the parameter of interest is likely to be close to the mean with the potential for large deviations, the Laplace distribution serves as a suitable prior distribution for such scenarios.
\index{Laplace distribution}
\begin{definition}[Laplace distribution]\label{definition:laplace_distribution}
A random variable $\rx$ is said to follow the \textit{Laplace distribution} with location and scale parameters $\mu$ and $b>0$, respectively, denoted by $\rx \sim \laplacedist(\mu,b)$, if 
$$
f(x; \mu,b)=\frac{1}{2b} \exp \left( -\frac{\abs{x-\mu}}{b } \right).
$$
The mean and variance of $\rx \sim \laplacedist( \mu,b)$ are given by 
\begin{equation}
\Exp[\rx] = \mu, \qquad \Var[\rx] =2b^2. \nonumber
\end{equation}
Figure~\ref{fig:dists_laplace} compares different parameters $\mu$ and $b$ for the Laplace distribution.
\end{definition}

%\begin{figure}[h]
%\centering  
%\vspace{-0.35cm} 
%\subfigtopskip=2pt 
%\subfigbottomskip=2pt 
%\subfigcapskip=-5pt 
%\subfigure[Laplace distribution.]{\label{fig:dists_laplace}
%	\includegraphics[width=0.481\linewidth]{./figures/dists_laplace.pdf}}
%\subfigure[Skew-Laplace distribution.]{\label{fig:dists_skewlaplace}
%	\includegraphics[width=0.481\linewidth]{./figures/dists_skewlaplace.pdf}}
%\caption{Laplace and skew-Laplace probability density functions for different values of the parameters.}
%\label{fig:dists_lap_and_skewlap}
%\end{figure}
\begin{SCfigure}%[H]
\centering
\includegraphics[width=0.5\textwidth]{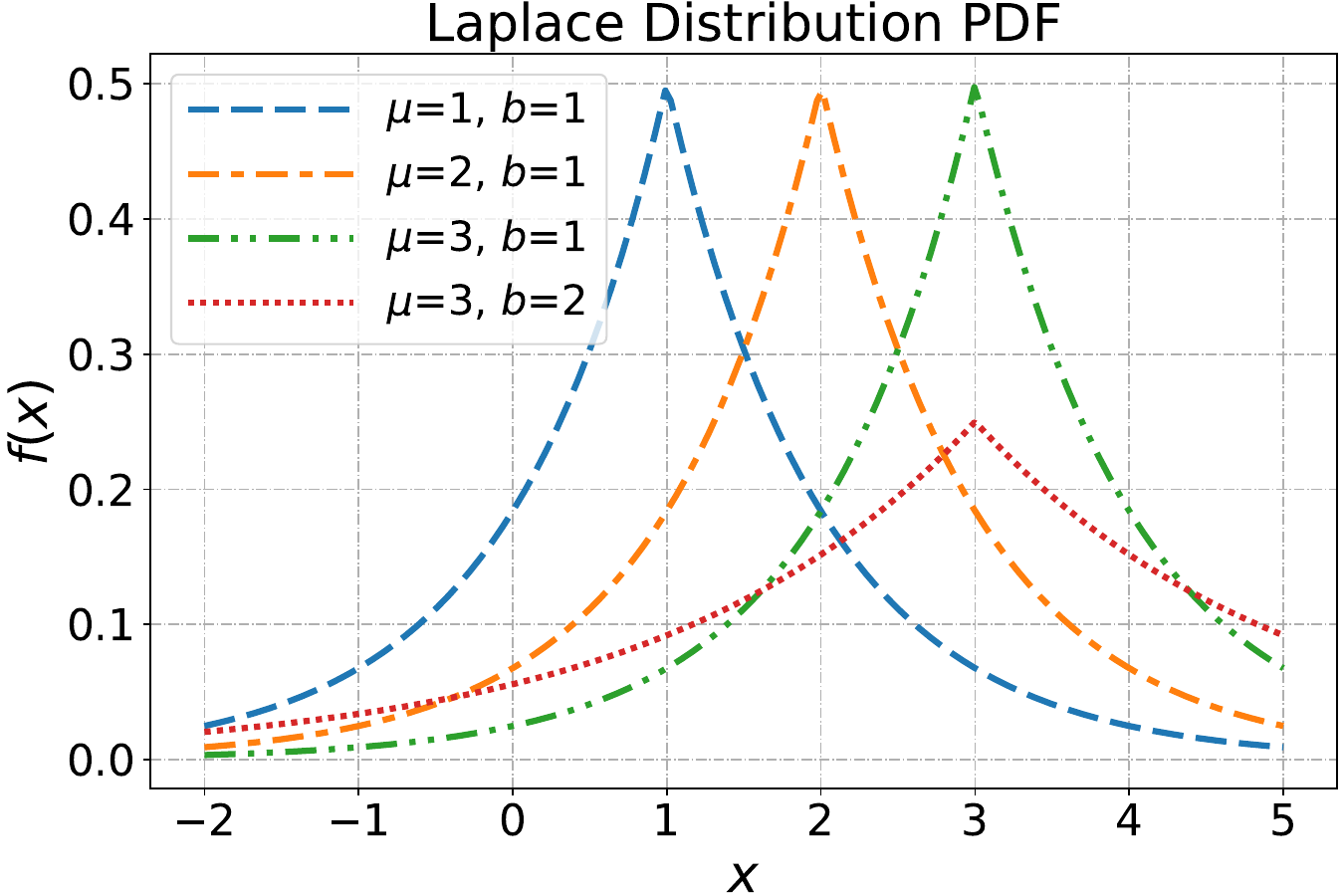}
\caption{Laplace  probability density functions for different values of the parameters.}
\label{fig:dists_laplace}
\end{SCfigure}

\index{Binomial distribution}
More often than not, we repeat an experiment multiple times independently with two alternative outcomes, say ``success" and ``failure", and we want to model the overall number of successes. 
We are inevitably taken to the binomial distribution if each experiment is modeled as a Bernoulli distribution. 
This simulates the overall proportion of heads in a run of $n$ separate coin flips.
\begin{definition}[Binomial distribution]\label{definition:binomial_distri}\index{Binomial distribution}
A random variable $\rx$ is said to follow the \textit{binomial distribution} with parameter $p\in (0,1)$ and $n \in \mathbb{N}$, denoted by $\rx \sim \binoimial(n,p)$, if 
$$ 
f(x; n, p)=\binom{n}{x} p^x(1-p)^{n-x},
$$
where $\binom{n}{x}$ is known as the \textit{binomial coefficient}.
The mean and variance of $\rx \sim \binoimial(a, b)$ are given by 
\begin{equation}
\Exp[\rx] = np, \qquad \Var[\rx] = np(1-p). \nonumber
\end{equation}
Figure~\ref{fig:dists_binimial} compares different parameters of $p$ with $n=10$ for the binomial distribution.
\end{definition}

A distribution that is closely related to the Binomial distribution is called the Bernoulli distribution.
A random variable is said to follow the \textit{Bernoulli} distribution with parameter $p\in(0,1)$, denoted as $\rx\sim\bernoullidist(p)$, if  
\begin{equation}\label{equation:bbern_dist}
f(x; p) =p\indicator\{x=1\} + (1-p)\indicator\{x=0\},
\end{equation}
with mean $\Exp[\rx]=p$ and variance $\Var[\rx]=p(1-p)$, respectively.

\begin{exercise}[Bernoulli and binomial]\label{exercise:bern_binom}
Show that if $\rx=\sum_{i=1}^{N} \ry_i$ with $\ry_i\stackrel{i.i.d.}{\sim} \bernoullidist(p)$, then we have $\rx\sim \binoimial(N, p)$.
\end{exercise}

\begin{exercise}[Scaled binomial]\label{exercise:scaled_binom}
When considering proportions instead of counts, we scale the binomial variable $\rx$ by dividing it by the number of trials $n$. Define $\ry$ as the random variable representing the proportion of successes:
$$
\ry = \frac{\rx}{n}
$$
Then, $\ry$ can be seen as having a ``scaled" or ``rescaled'' binomial distribution, with the PMF adjusted accordingly:
$$
\Pr(\ry = y) = \binom{n}{ny} p^{ny} (1-p)^{n(1-y)}
$$
for values of $y$ that are multiples of $1/n$, i.e., $z = 0, \frac{1}{n}, \frac{2}{n}, \dots, 1$.
Show that 
$$
\Exp[\ry] = p \qquad \text{and} \qquad \Var[\ry] = \frac{p(1-p)}{n}.
$$
\end{exercise}
The scaled binomial distribution is  particularly useful when modeling data that represent proportions or rates, such as the proportion of individuals responding to a treatment in clinical trials or the success rate of an event over multiple attempts.
In statistical modeling, especially within the framework of generalized linear models (GLMs), using the scaled binomial allows for direct modeling of probabilities or proportions while accounting for the number of trials. This approach simplifies the interpretation of model coefficients and predictions, focusing on the expected proportion of successes rather than the raw count; see Example~\ref{example:grouped_logis}.

\begin{SCfigure}%[H]
\centering
\includegraphics[width=0.5\textwidth]{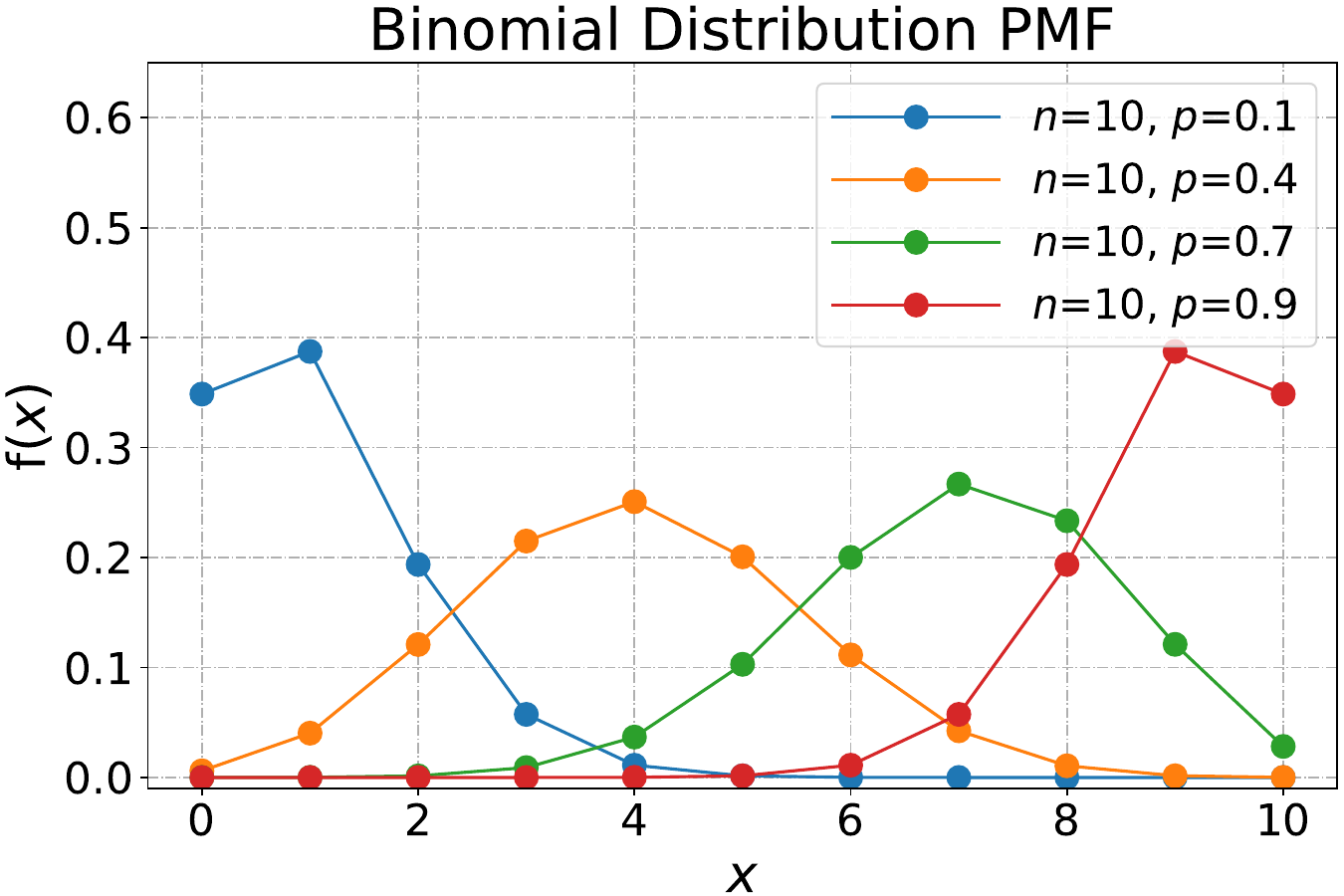}
\caption{Binomial distribution probability mass
functions for different values of the parameters $p$ with $n=10$.}
\label{fig:dists_binimial}
\end{SCfigure}

%\begin{figure}[h]
%	\centering
%	\includegraphics[width=0.55\textwidth]{figures/dists_binomial.pdf}
%	\caption{Binomial distribution probability mass
%		functions for different values of the parameters $p$ with $n=10$.}
%	\label{fig:dists_binimial}
%\end{figure}

\index{Exponential distribution}
The \textit{exponential} distribution is a  probability distribution commonly used in modeling events occurring  randomly over time, such as  the time elapsed until the occurrence of a certain event, or the time between two consecutive events.
It is a special Gamma distribution (see definition below) with support on nonnegative real values.
\begin{definition}[Exponential distribution]\label{definition:exponential_distribution}
A random variable $\rx$ is said to follow the \textit{exponential distribution} with rate parameter $\lambda>0$ \footnote{Note the inverse rate parameter $1/\lambda$ is called the scale parameter.
In probability theory and statistics, 
the \textit{location} parameter shifts the entire distribution left or right, e.g., the mean parameter of a Gaussian distribution; 
the \textit{shape} parameter compresses or stretches the entire distribution;
the \textit{scale} parameter changes the shape of the distribution in some manner.
}, denoted by $\rx \sim \exponential(\lambda)$, if 
$$ f(x; \lambda)=\left\{
\begin{aligned}
& \lambda \exp(-\lambda x)
,& \mathrm{\,\,if\,\,} x \geq 0;  \\
&0 , &\mathrm{\,\,if\,\,} x <0.
\end{aligned}
\right.
$$
We will see this is equivalent to $\rx\sim \gammadist(1, \lambda)$, a Gamma distribution.
The mean and variance of $\rx \sim \exponential(\lambda)$ are given by 
\begin{equation}
\Exp[\rx] = \lambda^{-1}, \qquad \Var[\rx] =\lambda^{-2}. \nonumber
\end{equation}
The support of an exponential distribution is on $(0,\infty)$.
Figure~\ref{fig:dists_exponential} compares different parameters $\lambda$ for the exponential distribution.
\end{definition}

Note that the average $\lambda^{-1}$ is the average time until the occurrence of the event of interest, interpreting $\lambda$   as a rate parameter. 
An important property of the exponential distribution is that it is ``memoryless," meaning that  the probability of waiting for an additional amount of time $x$  depends only on $x$, not on the past waiting time.
\begin{remark}[Property of exponential distribution]
Let $\rx\sim \exponential(\lambda)$. Then we have $\Pr(\rx\geq x + s \mid \rx \geq s) = \Pr(\rx\geq x)$.
\end{remark}
\begin{SCfigure}%[H]
\centering
\includegraphics[width=0.5\textwidth]{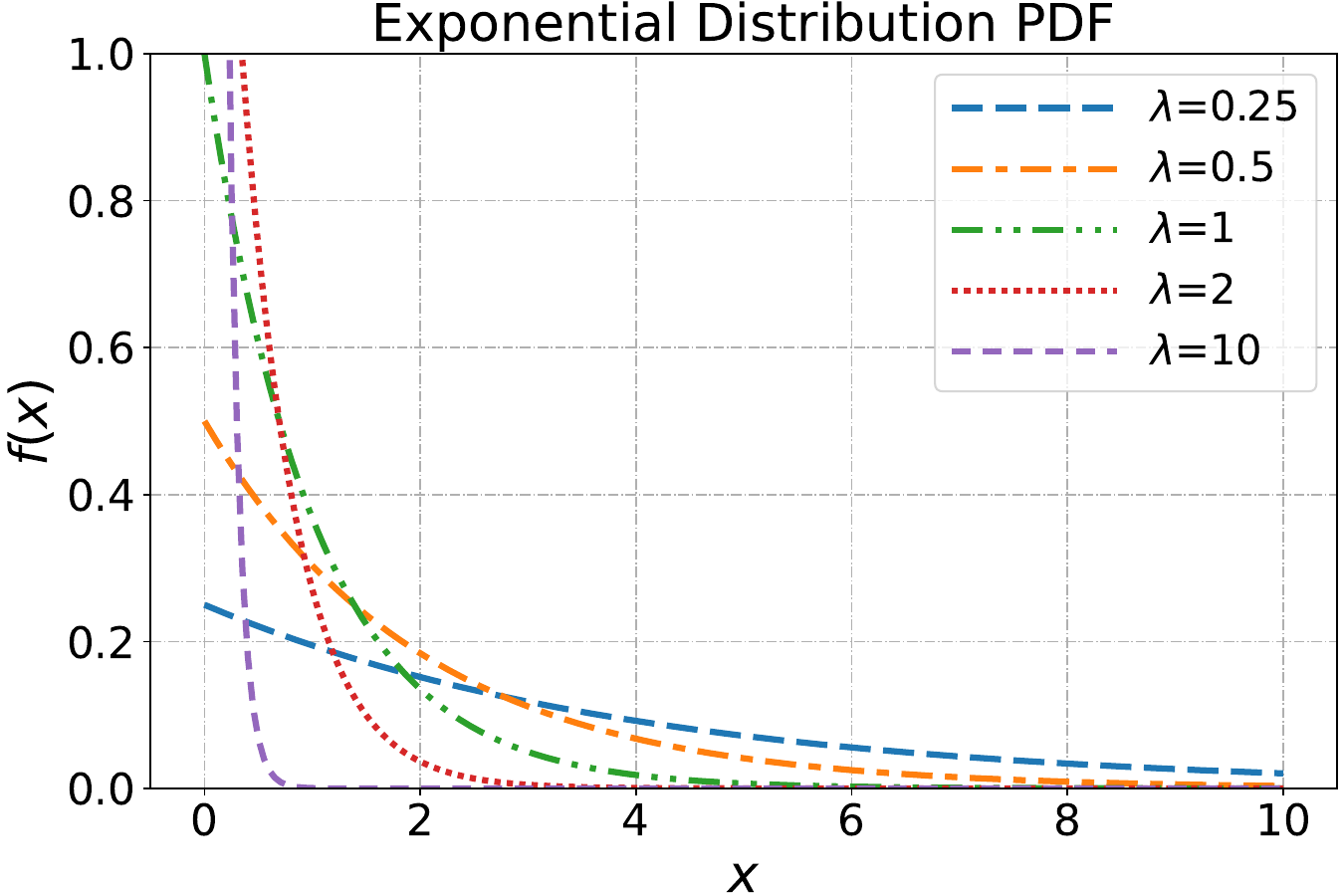}
\caption{Exponential probability density functions for different values of the rate parameter $\lambda$.}
\label{fig:dists_exponential}
\end{SCfigure}

\begin{definition}[Gamma distribution]\label{definition:gamma_distri}\index{Gamma distribution}
A random variable $\rx$ is said to follow the \textit{Gamma distribution} with shape parameter $r>0$ and rate parameter $\lambda>0$, denoted by $\rx \sim \gammadist(r, \lambda)$, if 

$$ f(x; r, \lambda)=\left\{
\begin{aligned}
&\frac{\lambda^r}{\Gamma(r)} x^{r-1} \exp(-\lambda x) ,& \mathrm{\,\,if\,\,} x \geq 0.  \\
&0 , &\mathrm{\,\,if\,\,} x <0,
\end{aligned}
\right.
$$
where $\Gamma(x)=\int_{0}^{\infty} t^{x-1}\exp(-t)dt$ is the \textit{Gamma function},  and we can just take it as a function to normalize the distribution into sum to 1. In special cases when $y$ is a positive integer, $\Gamma(y) = (y-1)!$.
We will delay the introduction of Chi-squared distribution in Definition~\ref{definition:chisquare_dist}. 
The mean and variance of $\rx \sim \gammadist(r, \lambda)$ are given by 
\begin{equation}
\Exp[\rx] = \frac{r}{\lambda}, \qquad \Var[\rx] = \frac{r}{\lambda^2}. \nonumber
\end{equation}
Figure~\ref{fig:gamma_chisquare_compare} compares different parameters for the Gamma distribution.
\end{definition}

It is important to note that the definition of the Gamma distribution does not constrain $r$ to be a natural number; instead, it allows  $r$ to take any positive value.
However, when $r$ is a positive integer, the Gamma distribution can be interpreted as a sum of $r$ exponentials of rate $\lambda$ (see Definition~\ref{definition:exponential_distribution}).
The summation property holds true more generally for Gamma variables with the same rate parameter. If $\rx_1$ and $\rx_2$ are random variables drawn from $\gammadist(r_1, \lambda)$ and $\gammadist(r_2, \lambda)$, respectively, then their sum $\rx_1+\rx_2$ is a Gamma random variable from $\gammadist(r_1+r_2, \lambda)$.

\index{Integration by parts}
In the Gamma distribution definition, we observe that the Gamma function can be defined as follows:
$$
\Gamma(y) = \int_{0}^{\infty} x^{y-1} e^{-x} dx,  \qquad y\geq 0.
$$
Utilizing integration by parts $\int_{a}^{b} u(x) v^\prime(x) dx = u(x)v(x)|_a^b - \int_a^b u^\prime(x) v(x)dx$, where $u(x) =x^{y-1}$ and $v(x)=-e^{-x}$, we derive 
$$
\begin{aligned}
\Gamma(y) &= -x^{y-1}e^{-x}|_0^{\infty} - \int_{0}^{\infty} (y-1)x^{y-2}(-e^{-x}) dx \\
&= 0+ (y-1) \int_{0}^{\infty}x^{y-2}e^{-x}dx 
= (y-1)\Gamma(y-1).
\end{aligned}
$$ 
This demonstrates  that when $y$ is a positive integer, the relationship $\Gamma(y) = (y-1)!$ holds true.

\begin{figure}[h]
\centering
\vspace{-0.35cm}
\subfigtopskip=2pt
\subfigbottomskip=2pt
\subfigcapskip=-5pt
\subfigure[Gamma probability density
functions for different values of the parameters $r$ and $\lambda$.]{\label{fig:dists_gamma}
\includegraphics[width=0.47\linewidth]{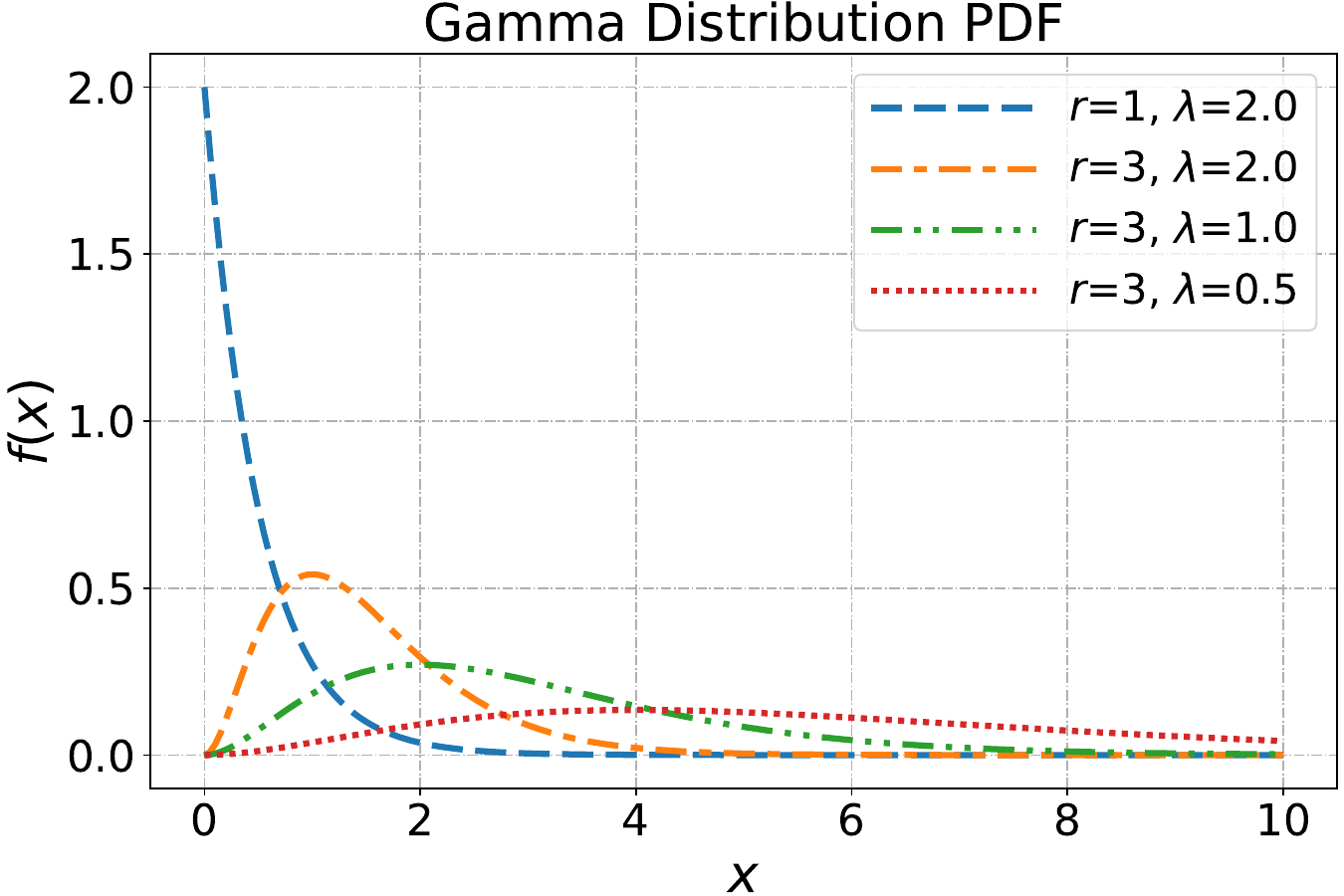}}
\quad 
\subfigure[Chi-squared probability density
functions for different values of the parameter $p$.]{\label{fig:dists_chisquare2}
\includegraphics[width=0.47\linewidth]{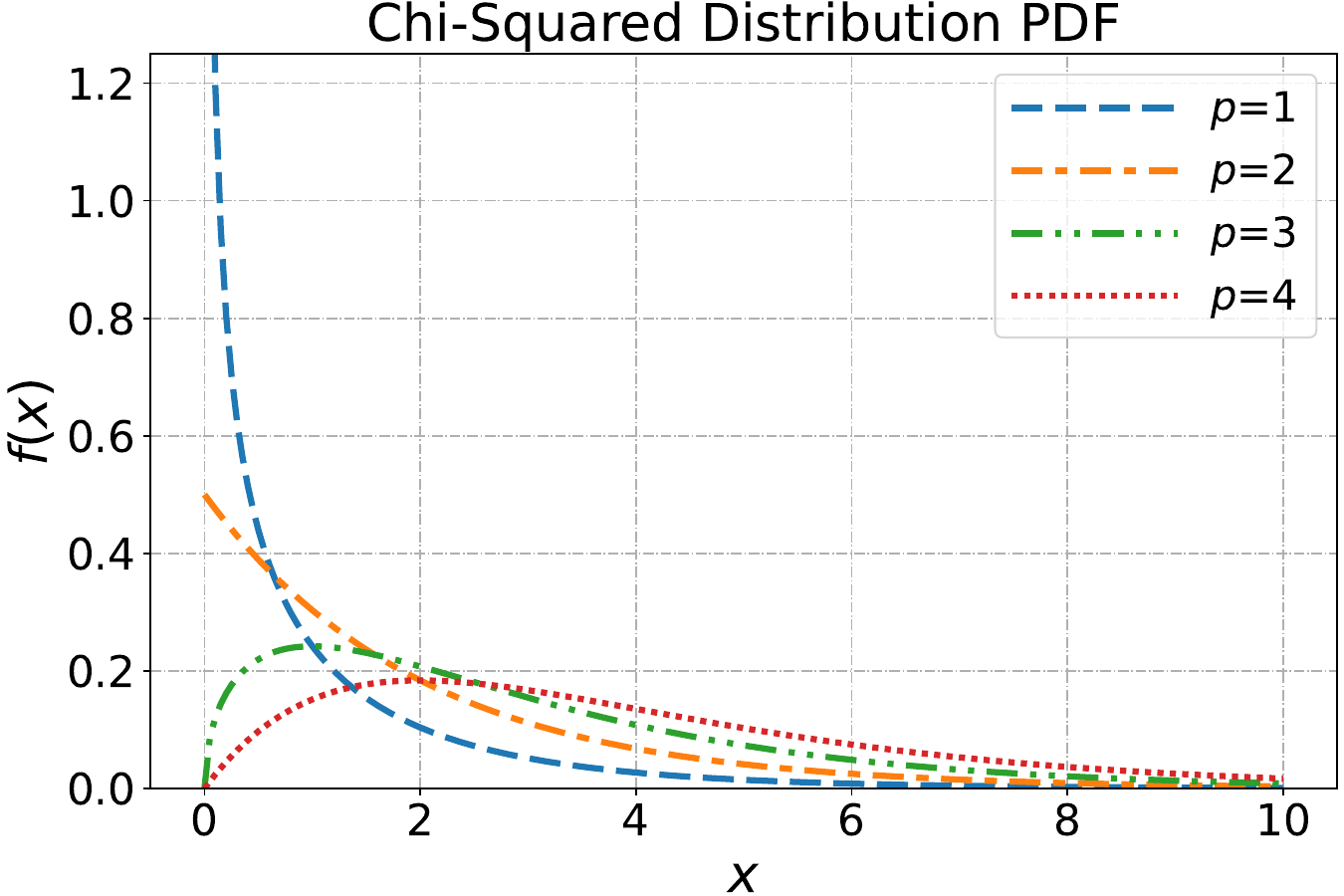}}
\caption{Comparison between the Gamma distribution and the Chi-squared distribution.}
\label{fig:gamma_chisquare_compare}
\end{figure}

Though the Chi-squared distribution is a special case of the Gamma distribution, it holds  particular significance in statistical theory.
\begin{definition}[Chi-squared distribution, $\chi^2$-Distribution]\label{definition:chisquare_dist}\index{Chi-squared distribution}
Let $\ra_i\sim \normal(0, 1)$ for $i\in\{1,2,\ldots,p\}$ (equivalently, $\rva \sim \normal(\bzero, \bI_{p})$; see Definition~\ref{definition:multivariate_gaussian}). 
Then, $\rx=\sum_{i=1}^p \ra_{i}^2$ follows the \textit{Chi-squared distribution} (or \textit{Chi-square distribution}, \textit{$\chi^2$-distribution}) with \textit{$p$ degrees of freedom}. 
We write $\rx \sim \chi_{(p)}^2$, and it is equivalent to $\rx\sim \gammadist(p/2, 1/2)$ in Definition~\ref{definition:gamma_distri}. The probability density function is given by 
$$ f(x; p)=\left\{
\begin{aligned}
	&\frac{1}{2^{p/2}\Gamma(\frac{p}{2})} x^{\frac{p}{2}-1} \exp(-\frac{x}{2}) ,& \mathrm{\,\,if\,\,} x \geq 0;  \\
	&0 , &\mathrm{\,\,if\,\,} x <0.
\end{aligned}
\right.
$$
The mean, variance of $\rx\sim \chi_{(p)}^2$ are given by 
$$
\Exp[\rx]=p, \qquad\Var[\rx]=2p.
$$
Figure~\ref{fig:dists_chisquare2} compares different parameters $p$ for the Chi-squared distribution.
\end{definition}
%\begin{figure}[h!]
%\centering
%\includegraphics[width=0.60\textwidth]{figures/dists_chisquared.pdf}
%\caption{Chi-squared probability density
%	functions for different values of the parameter $p$.}
%\label{fig:dists_chisquare}
%\end{figure}

The definition shows that if $\rva=[\ra_1, \ra_2, \ldots, \ra_p]^\top \sim \normal(\bzero, \bI_{p})$ (the multivariate Gaussian distribution; as a hindsight, see Definition~\ref{definition:multivariate_gaussian}), then $\rx=\rva^\top\rva\sim \chi_{(p)}^2$. 
Given an orthogonal projection matrix $\bH$ of rank $r<p$ (as a hindsight, see Definition~\ref{definition:orthogonal-projection-matrix}), then it follows that 
\begin{equation}
\rva^\top \bH \rva \sim \chi_{(r)}^2, \text{ with orthogonal projector $\rank(\bH)=r<p$}.
\end{equation}
Suppose that $ \rvx \sim \normal(\bmu, \bSigma) $  where $\bmu\in\real^p$ and $\bSigma^{-1}$ is nonsingular. Then,
\begin{equation}
(\rvx - \bmu)^\top \bSigma^{-1} (\rvx - \bmu) \sim \chi^2_{(p)}.
\end{equation}
To see this, we note that $\bSigma$ is symmetric and positive definite. Then, we can write that $\bSigma = \bQ^\top \bLambda \bQ$, where $\bQ$ is an $p \times p$ orthogonal matrix and $\bLambda = \diag(\lambda_1,\lambda_2, \ldots, \lambda_p)$ with $\lambda_j > 0$ (as a hindsight, see Theorem~\ref{theorem:spectral_theorem}).
Then, define $\rvy \triangleq \bLambda^{-1/2} \bQ (\rvx - \bmu)$, which is a standardized version of $\rvx$. Vector $\rvy$ then is normally distributed by Lemma~\ref{lemma:affine_mult_gauss}:
$$
\rvy \sim \normal(\bzero, \bLambda^{-1/2} \bQ \bSigma \bQ^\top \bLambda^{-1/2}) = \normal(\bzero, \bLambda^{-1/2} \bQ \bQ^\top \bLambda \bLambda^{-1/2}) = \normal(\bzero, \bI_p).
$$
It follows that the elements of $\rvy$ are independent and that $\ry_i \sim \normal(0, 1)$. So,
$(\rvx - \bmu)^\top \bSigma^{-1} (\rvx - \bmu) = \rvy^\top\rvy \sim \chi^2_{(p)}$.

\begin{definition}[Inverse-Gamma distribution]\label{definition:inverse_gamma_distribution}
A random variable $\rx$ is said to follow the \textit{inverse-Gamma distribution} with shape parameter $r>0$ and scale parameter $\lambda>0$, denoted by $\rx\sim \inversegammadist(r, \lambda)$, if
$$ f(x; r, \lambda)=\left\{
\begin{aligned}
&\frac{\lambda^r}{\Gamma(r)} x^{-r-1} \exp(- \frac{\lambda}{x} ) ,& \mathrm{\,\,if\,\,} x > 0;  \\
&0 , &\mathrm{\,\,if\,\,} x \leq 0.
\end{aligned}
\right.
$$
The mean and variance of inverse-Gamma distribution are given by 
$$ \Exp[\rx]=\left\{
\begin{aligned}
&\frac{\lambda}{r-1}, \, &\mathrm{if\,} r\geq 1; \\
&\infty, \, &\mathrm{if\,} 0<r<1.
\end{aligned}
\right.\qquad
\Var[\rx]=\left\{
\begin{aligned}
&\frac{\lambda^2}{(r-1)^2(r-2)}, \, &\mathrm{if\,} r> 2; \\
&\infty, \, &\mathrm{if\,} 0<r\leq 2.
\end{aligned}
\right.
$$
Figure~\ref{fig:dists_inversegamma} illustrates the impact of different parameters $r$ and $\lambda$ for the inverse-Gamma distribution.
\end{definition}
If $\rx$ is Gamma distributed, then $\ry=1/\rx$ is inverse-Gamma distributed.
Note that the inverse-Gamma density is not simply the Gamma density with
$x$ replaced by $\frac{1}{y}$. There is an additional factor of $y^{-2}$.~\footnote{Which is from the \textit{Jacobian in the change-of-variables formula}. A short proof is provided here. Let $y=\frac{1}{x}$ where $y\sim \inversegammadist(r, \lambda)$ and $x\sim \gammadist(r, \lambda)$. Then, $f(y) |dy| = f(x) |dx|$, which results in $f(y) = f(x) \abs{\frac{dx}{dy}} = f(x)x^2 \xlongequal{ \mathrm{y}=\frac{1}{x}} \frac{\lambda^r}{\Gamma(r)} y^{-r-1} \exp(- \frac{\lambda}{y})$ for $y>0$. }  The inverse-Gamma distribution is useful as a prior for positive parameters. It imparts a quite heavy tail and keeps probability further from zero than the Gamma distribution (see examples in Figure~\ref{fig:dists_inversegamma}).

%\begin{definition}[Inverse-Gamma Distribution]
%	A random variable $Y$ is said to follow the inverse-Gamma distribution with parameters $r>0$ and $\lambda>0$ if
%	
%	$$ f(y; r, \lambda)=\left\{
%	\begin{aligned}
%		&\frac{\lambda^r}{\Gamma(r)} y^{-r-1} \exp(- \frac{\lambda}{y} ) ,& \mathrm{\,\,if\,\,} y > 0.  \\
%		&0 , &\mathrm{\,\,if\,\,} y \leq 0,
%	\end{aligned}
%	\right.
%	$$
%	where again $\Gamma(\cdot)$ is the Gamma function.
%	And it is denoted by $Y \sim \inversegammadist(r, \lambda)$.
%	The mean and variance of inverse-Gamma distribution are given by 
%	$$ \Exp[Y]=\left\{
%	\begin{aligned}
%		&\frac{\lambda}{r-1}, \, &\mathrm{if\,} r\geq 1. \\
%		&\infty, \, &\mathrm{if\,} 0<r<1.
%	\end{aligned}
%	\right.\qquad
%	\Var[Y]=\left\{
%	\begin{aligned}
%		&\frac{\lambda^2}{(r-1)^2(r-2)}, \, &\mathrm{if\,} r\geq 2. \\
%		&\infty, \, &\mathrm{if\,} 0<r<2.
%	\end{aligned}
%	\right.
%	$$
%	Figure~\ref{fig:gamma_inversee_compare} compares different parameters for Gamma distribution and inverse-Gamma distributions.
%\end{definition}

\index{Gamma distribution}
\index{Inverse-Gamma distribution}
\begin{figure}[h]
\centering
\vspace{-0.35cm}
\subfigtopskip=2pt
\subfigbottomskip=2pt
\subfigcapskip=-5pt
\subfigure[Inverse-Gamma probability density
functions for different values of the parameters $r$ and $\lambda$.]{\label{fig:dists_inversegamma}
\includegraphics[width=0.481\linewidth]{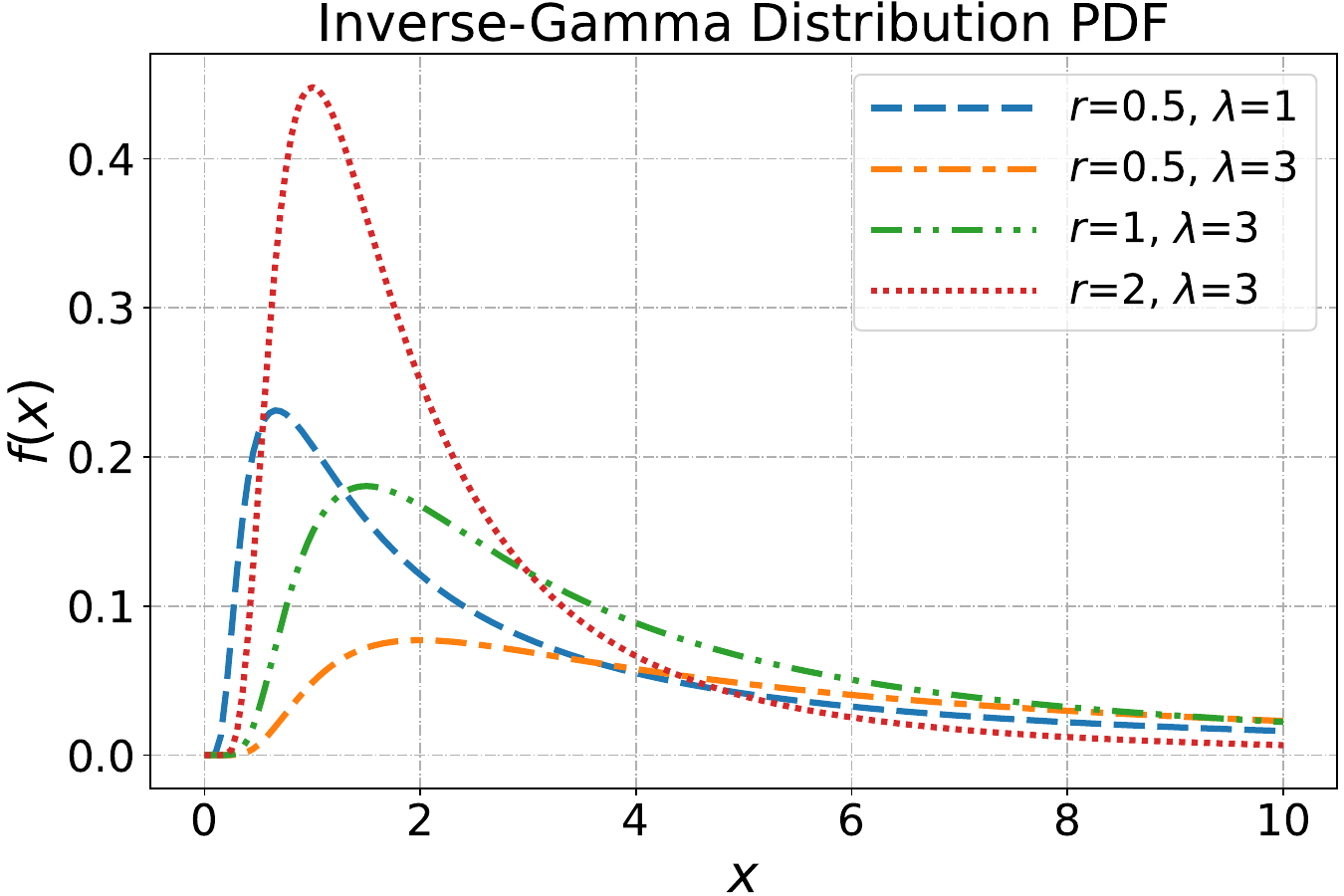}}
\subfigure[Inverse-Chi-squared probability density
functions for different values of the parameters $v$ and $s^2$.]{\label{fig:dist_inversechisquared}
\includegraphics[width=0.481\linewidth]{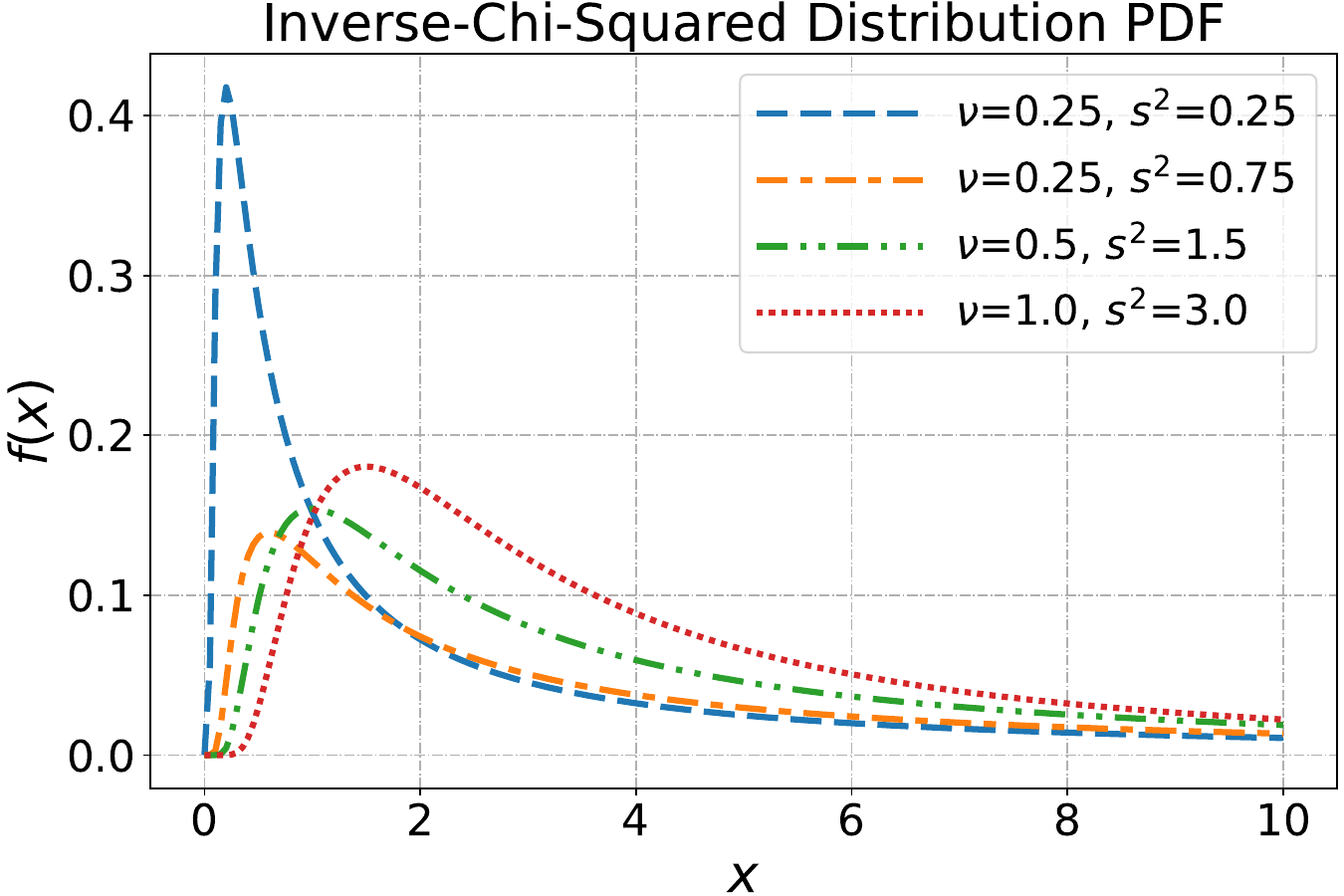}}
\caption{Comparison between the inverse-Gamma distribution and inverse-Chi-squared distribution for different values of the parameters.}
\label{fig:gamma_inversee_compare}
\end{figure}

\begin{definition}[Inverse-Chi-squared distribution]\label{definition:inverse-chi-square}
A random variable $\rx$ is said to follow the \textit{inverse-Chi-squared distribution} with parameter $\nu>0$ and $s^2>0$, denoted by $\rx\sim \inversegammadist(\frac{\nu}{2}, \frac{\nu s^2}{2})$, if
$$ f(x; \nu, s^2)=\left\{
\begin{aligned}
&\frac{{(\frac{\nu s^2}{2})}^{\frac{\nu}{2}}}{\Gamma(\frac{\nu}{2})} x^{-\frac{\nu}{2}-1} \exp(- \frac{\nu s^2}{2x} ) ,& \mathrm{\,\,if\,\,} x > 0;  \\
&0 , &\mathrm{\,\,if\,\,} x \leq 0.
\end{aligned}
\right.
$$
And it is also compactly denoted by $\rx \sim \inversechidist(\nu, s^2)$. The parameter $\nu >0$ is called the \textit{degrees of freedom}, and $s^2 > 0$ is the \textit{scale parameter}. And it is also known as the \textit{scaled} inverse-Chi-squared distribution.
The mean and variance of the inverse-Chi-squared distribution are given by 
$$ \Exp[\rx]=\left\{
\begin{aligned}
&\frac{\nu s^2}{\nu-2}, \, &\mathrm{if\,\,} \nu\geq 2; \\
&\infty, \, &\mathrm{if\,\,} 0<\nu<2.
\end{aligned}
\right.\qquad
\Var[\rx]=\left\{
\begin{aligned}
&\frac{2\nu^2 s^4}{(\nu-2)^2(\nu-4)}, \, &\mathrm{if\,\,} \nu\geq 4; \\
&\infty, \, &\mathrm{if\,\,} 0<\nu<4.
\end{aligned}
\right.
$$
To establish a connection with the inverse-Gamma distribution, we can set $S=\nu s^2$. Then the inverse-Chi-squared distribution can also be denoted by $\rx\sim \inversegammadist(\frac{\nu}{2}, \frac{S}{2})$ if $\rx \sim \inversechidist(\nu, s^2)$,  the form of which conforms to the univariate case of the inverse-Wishart distribution (see \citet{lu2023bayesian}). 
Figure~\ref{fig:dist_inversechisquared} illustrates the impact of different parameters $\nu$ and $s^2$ for the inverse-Chi-squared distribution.
\end{definition}

\index{Beta distribution}
\begin{definition}[Beta distribution]\label{definition:beta_distri}\index{Beta distribution}
A random variable $\rx$ is said to follow the \textit{Beta distribution} with parameter $a>0$ and $b>0$, denoted by $\rx \sim \betadist(a,b)$, if 
$$ 
f(x; a, b)=\left\{
\begin{aligned}
&\frac{1}{B(a,b)} x^{a-1}(1-x)^{b-1} ,& \mathrm{\,\,if\,\,} 0 \leq x \leq 1.  \\
&0 , &\mathrm{\,\,otherwise\,\,},
\end{aligned}
\right.
$$
where $B(a,b)$ denotes  \textit{Euler's Beta function} and it can be seen as a normalization term. Equivalently, $B(a,b)$ can be obtained by 
$$
B(a,b) = \frac{\Gamma(a)\Gamma(b)}{\Gamma(a+b)},
$$
where $\Gamma(\cdot)$ is the Gamma function.
The mean and variance of $\rx \sim \betadist(a, b)$ are given by 
\begin{equation}
\Exp[\rx] = \frac{a}{a+b}, \qquad \Var[\rx] = \frac{ab}{(a+b+1)(a+b)^2}. \nonumber
\end{equation}
Figure~\ref{fig:dists_beta} compares different parameters of $a$ and $b$ for  the Beta distribution. When $a=b=1$, the Beta distribution reduces to a uniform distribution in the range of 0 and 1; see Exercise~\ref{exercise:uniform_dist}.
\end{definition}

\begin{SCfigure}%[H]
\centering
\includegraphics[width=0.5\textwidth]{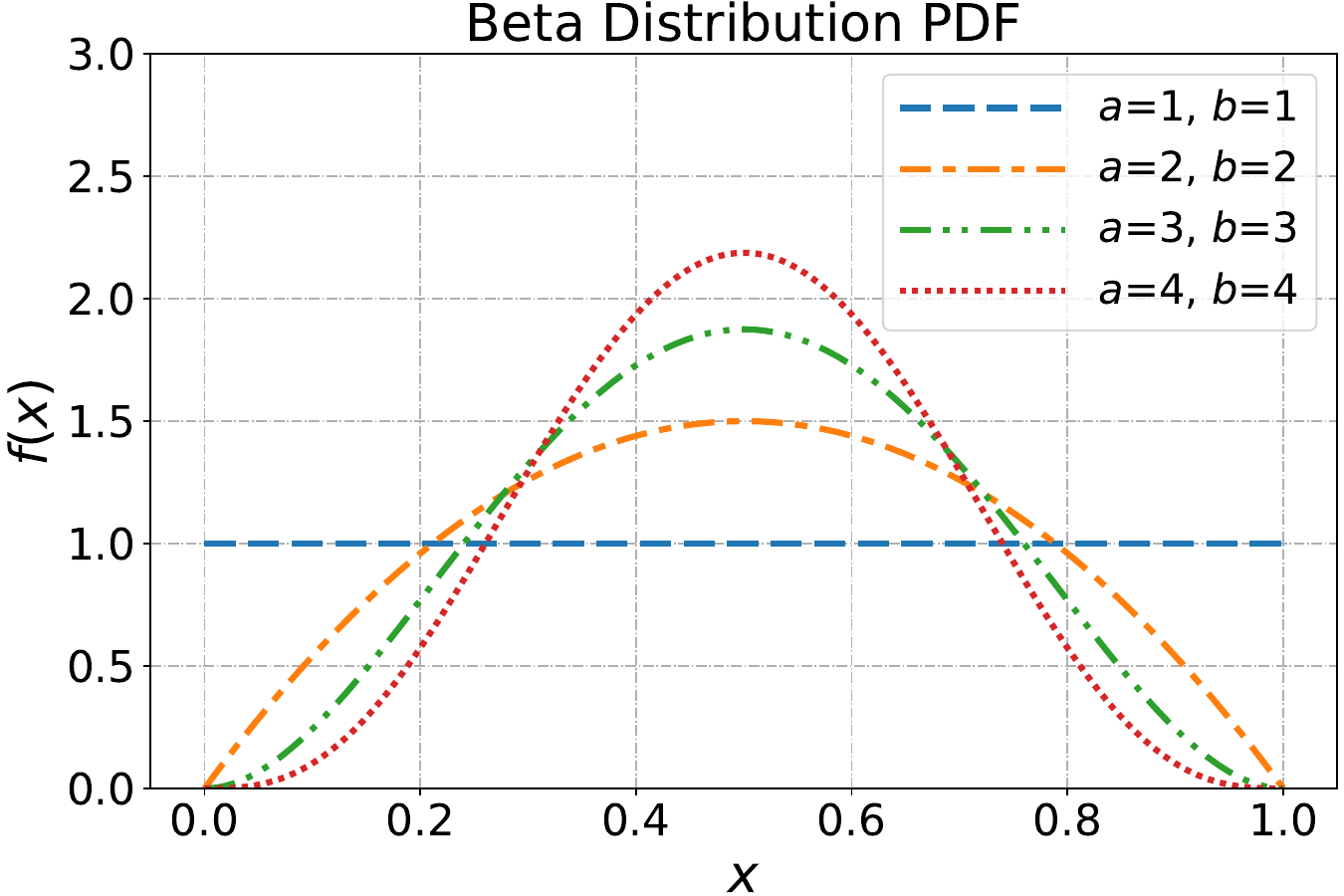}
\caption{Beta distribution probability density
functions for different values of the parameters $a$ and $b$. When $a=b=1$, the Beta distribution reduces to a uniform distribution in the range of 0 and 1.}
\label{fig:dists_beta}
\end{SCfigure}

The \textit{Poisson} distribution is a discrete probability distribution that characterizes the number of events in a fixed interval of time or space, given the average number of events in that interval.
The Poisson distribution is frequently employed for modeling  count data, such as the number of calls received by a call center in an hour or the number of emails received in a day provided that the probability of a ``success" for any given instance is ``very small." 
To name more examples where the Poisson distribution can be applied, e.g., the number of stars in a random area of the space; the distribution of bacteria on a surface; the number of typographical errors on a typed page; the number of wrong connections to a phone number. 

\begin{definition}[Poisson Distribution]\label{definition:poisson_distribution}
A random variable $\rx\in\{0,1,2,3,\ldots\}$ is said to follow the \textit{Poisson distribution} with rate parameter $\lambda>0$,  denoted by $\rx \sim \poissondist(\lambda)$, if 
$$
f(x; \lambda)= \frac{\lambda^x}{x!}  \exp(-\lambda).
$$
The mean and variance of $\rx \sim \poissondist( \lambda)$ are given by 
\begin{equation}
\Exp[\rx] = \lambda, \qquad \Var[\rx] =\lambda. \nonumber
\end{equation}
The support of an exponential distribution is on $\{0,1,2,3,\ldots\} = \{0\}\cup \naturalset$.
Figure~\ref{fig:dists_poisson} compares probability mass functions of different parameter values $\lambda$ for the Poisson distribution.
\end{definition}

The mean and variance of the Poisson distribution are equal.
Roughly speaking, a Poisson distribution is the limit of a binomial distribution when $n\rightarrow \infty$ and $\pi=\lambda/n$, i.e., the number of trials diverges to infinity but the probability of success decreases to zero linearly with respect to the number of trials. This is also known as the \textit{law of rare events}.
Therefore, the Poisson distribution is often employed to model rare events like radioactive decays.

%The definition of the Poisson distribution also reveals its properties of a Poisson process:
%\begin{itemize}
%\item The average number of objects in any region $\sS\subseteq \sT$ is proportional to the size of $\sS$
%\end{itemize}

\begin{SCfigure}%[H]
\centering
\includegraphics[width=0.5\textwidth]{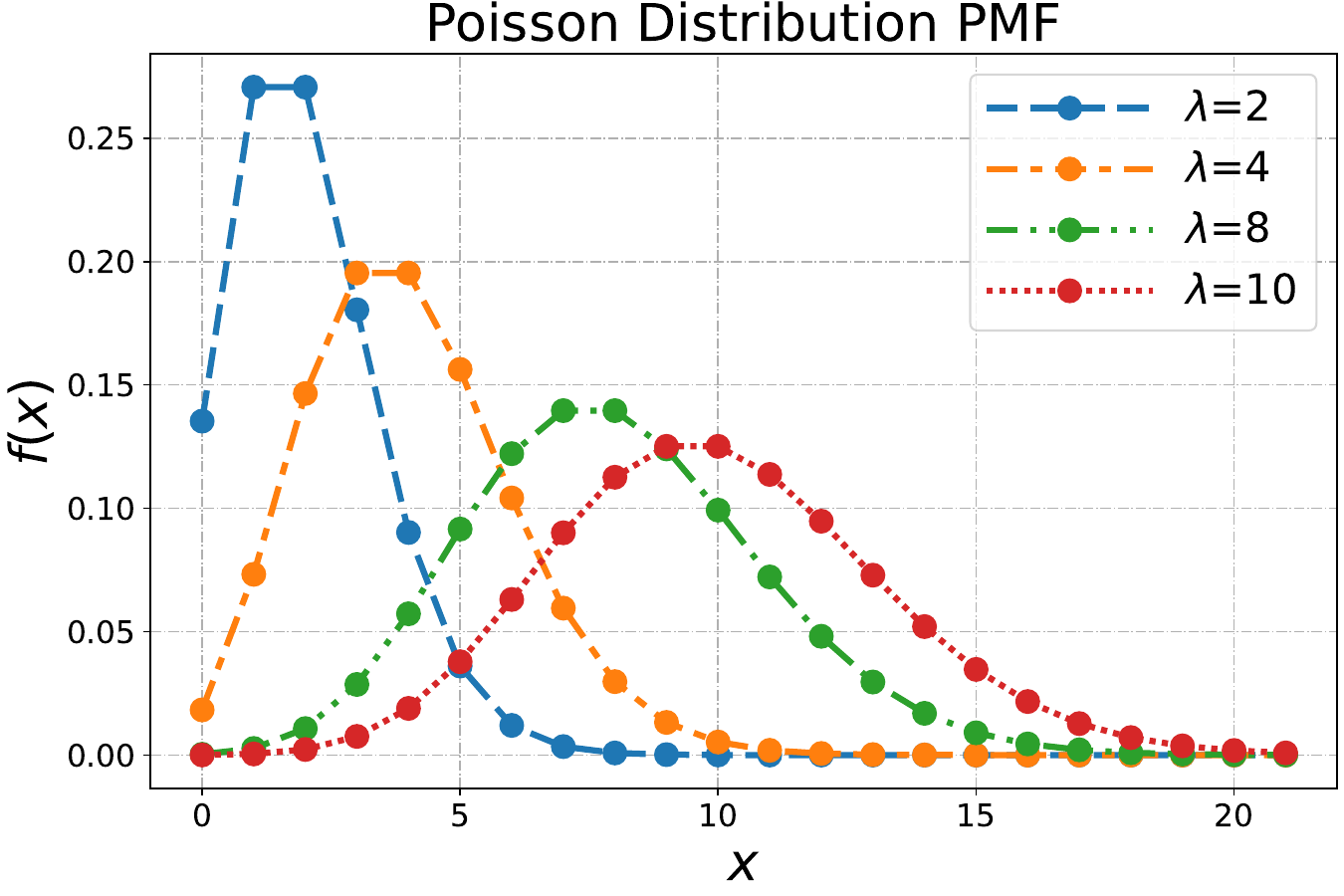}
\caption{Poisson probability mass functions for different values of the parameter $\lambda$.}
\label{fig:dists_poisson}
\end{SCfigure}

The sum of independently identical Poisson distributed random variables again follows a Poisson distribution.
\begin{theoremHigh}[Sum of Independently Distributed Poisson]\label{theorem:sum_iid_poisson}
Let $\rx_i\sim \poissondist(\lambda_i)$ for $i\in\{1,2,\ldots,n\}$. Then $ \ry=\sum_{i=1}^{n} \rx_i\sim \poissondist(\sum_{i=1}^{n}\lambda_i)$.
\end{theoremHigh}
For simplicity, we consider two independent Poisson random variables $\rx\sim \poissondist(\lambda_1)$ and $\ry\sim\poissondist(\lambda_2)$.
Define $\lambda\triangleq\lambda_1+\lambda_2$ and $\rz\triangleq\rx+\ry$. Then $\rz$ is a Poisson random variable with parameter $\lambda$. To see this, we have 
$$
\begin{aligned}
p(z) 
&= P(\rz=z) = \sum_{k=1}^{z} P(\rx=k) \cdot  P(\ry=z-k)
= \sum_{k=1}^{z} \frac{\lambda_1^k}{k!} \exp(-\lambda_1) \cdot \frac{\lambda_2^{z-k}}{(z-k)!} \exp(-\lambda_2)\\
&= \frac{\exp(-\lambda_1-\lambda_2)}{z!} \sum_{k=1}^{z} {z\choose k} \lambda_1^k\lambda_2^{z-k}
\stackrel{\dag}{=}\frac{\exp(-\lambda)}{z!}(\lambda_1+\lambda_2)^z = \frac{\lambda^z}{z!} \exp(-\lambda),
\end{aligned}
$$
where the equality ($\dag$) follows from the \textit{binomial theorem}.
Working for general, once we know  the sum of two Poisson random variables, we can keep adding more and more of them to obtain another Poisson variable.

\begin{theoremHigh}[Poisson and Multinomial]\label{theorem:multinomial_poisson}
Let $\rx_i\sim \poissondist(\lambda_i)$ be independent for $i\in\{1,2,\ldots, K\}$. Then the conditional distribution of $\rvx=[\rx_1,\rx_2,\ldots, \rx_k]^\top$ given $\sum_{i=1}^{K}\rx_i=N$ is $\multinomial_K(N,\{ p_1, p_2, \ldots,p_K\})$ with
$$
p_i= \frac{\lambda_i}{\lambda_1+\lambda_2+\ldots+\lambda_K}, \gap \text{for all }i\in\{1,2,\ldots,K\}.
$$
\end{theoremHigh}

\subsection{Common Multivariate Probability Distributions}\label{sec:multi_gaussian_conjugate_prior}

In this section, we further provide common multivariate probability distributions.
\subsection*{Multivariate Gaussian Distribution}\label{section:multi_gaussian_dist}

A \textit{multivariate Gaussian distribution} (also referred to as a \textit{multivariate normal distribution} or simply Gaussian distribution) is a continuous probability distribution  characterized by a  jointly normal distribution across multiple variables. 
It is fully described by its mean vector (of size equal to the number of variables) and covariance matrix (a square matrix of size equal to the number of variables). 
The covariance matrix encodes the pairwise relationships between variables in terms of the covariance between them. 
Widely applied in diverse domains like machine learning, statistics, and signal processing, the multivariate Gaussian proves (or simply called Gaussian when it's clear from the context) instrumental in modeling complex data distributions.
We first present the rigorous definition of the multivariate Gaussian distribution as follows.
\begin{definition}[Multivariate Gaussian distribution]\label{definition:multivariate_gaussian}
A random vector $\rvx \in \real^D$ is said to follow the \textit{multivariate Gaussian distribution (multivariate normal, MVN)} with parameters $\bmu\in\real^D$ and $\bSigma\in\real^{D\times D}$, denoted by $\rvx\sim \normal(\bmu, \bSigma)$, if
$$
\begin{aligned}
f(\bx; \bmu, \bSigma)&= (2\pi)^{-D/2} \abs{\bSigma}^{-1/2}\exp\left\{-\frac{1}{2}(\bx - \bmu)^\top \bSigma^{-1}(\bx - \bmu)\right\},~\footnote{The form of which can be proved using the moment generating function of $D$ i.i.d. univariate standard Gaussian variables.}
\end{aligned}
$$
where $\bmu \in \real^D$ is called the \textit{mean vector}, and $\bSigma\in \real^{D\times D}$ is positive definite and is called the \textit{covariance matrix}. $\abs{\bSigma} = \det(\bSigma)$ is the determinant of the matrix $\bSigma$.
The mean, mode, and covariance of the multivariate Gaussian distribution are given by 
\begin{equation*}
\begin{aligned}
\Exp [\rvx] &= \bmu, \qquad 
\mathrm{Mode}[\rvx] = \bmu, \qquad\text{and}\qquad 
\Cov [\rvx] = \bSigma. 
\end{aligned}
\end{equation*}
The covariance matrix can be obtained by 
$$
\Cov[\rvx]=\Exp[(\rvx-\bmu)(\rvx-\bmu)^\top]=\Exp[\rvx\rvx^\top]-\bmu\bmu^\top.
$$
Figure~\ref{fig:multi_gaussian_density} compares Gaussian density plots for different kinds of covariance matrices.
The multivariate Gaussian variable can be drawn from a univariate Gaussian density; see Problem~\ref{problem:multiGauss}.
\end{definition}

\begin{figure}[h]
\subfigure[Gaussian, $\bSigma =\begin{bmatrix}
1&0\\
0&1
\end{bmatrix}. $ ]{\includegraphics[width=0.31
\textwidth]{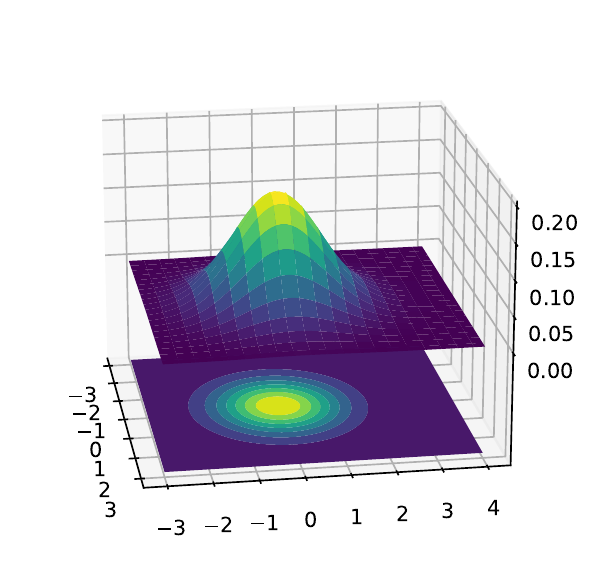} \label{fig:dists_multiGauss_sigma1}}
\subfigure[Gaussian, $\bSigma =\begin{bmatrix}
1&0\\
0&3
\end{bmatrix}.$]{\includegraphics[width=0.31
\textwidth]{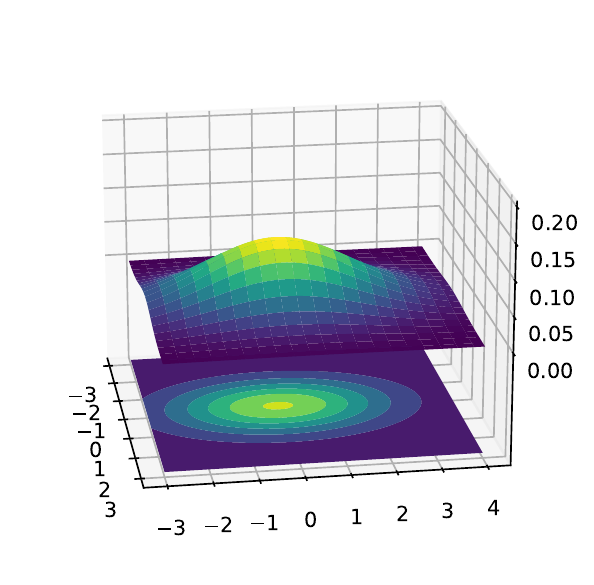} \label{fig:dists_multiGauss_sigma2}}
\subfigure[Gaussian, $\bSigma =\begin{bmatrix}
1&\textendash0.5\\
\textendash0.5&1.5
\end{bmatrix}.$]{\includegraphics[width=0.31 
\textwidth]{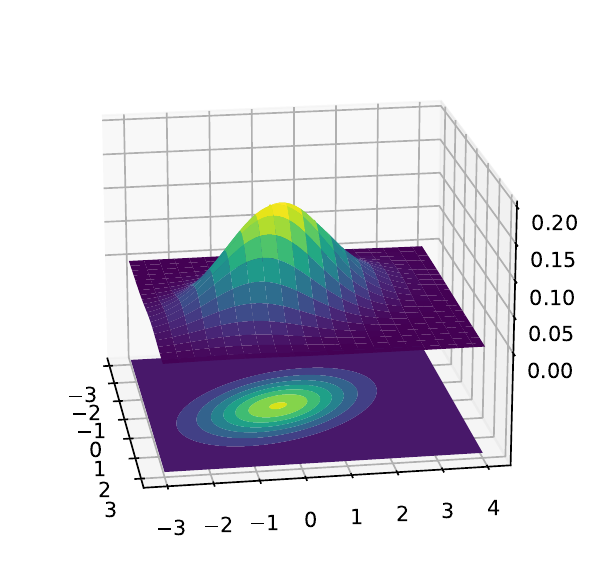} \label{fig:dists_multiGauss_sigma3}}
\subfigure[Gaussian, $\bSigma =\begin{bmatrix}
2&0\\
0&2
\end{bmatrix}. $ ]{\includegraphics[width=0.31
\textwidth]{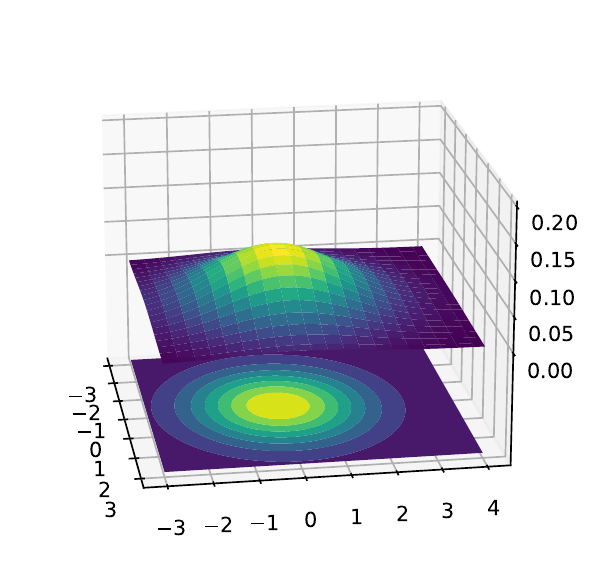} \label{fig:dists_multiGauss_sigma4}}
\subfigure[Gaussian, $\bSigma =\begin{bmatrix}
3&0\\
0&1
\end{bmatrix}.$]{\includegraphics[width=0.31
\textwidth]{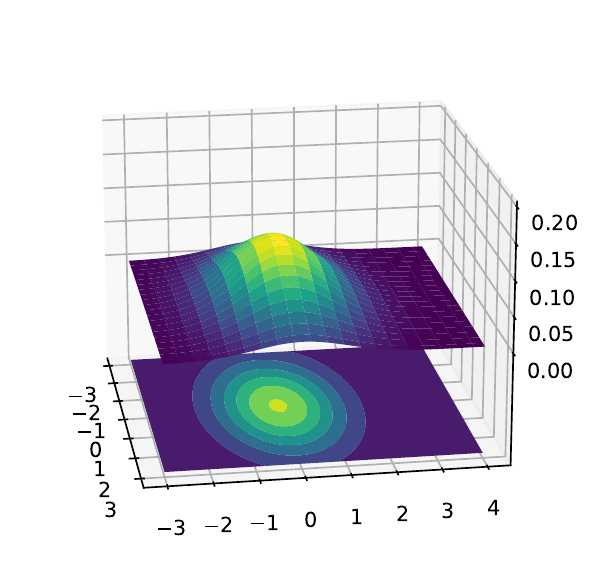} \label{fig:dists_multiGauss_sigma5}}
\subfigure[Gaussian, $\bSigma =\begin{bmatrix}
3&\textendash0.5\\
\textendash0.5&1.5
\end{bmatrix}.$]{\includegraphics[width=0.31
\textwidth]{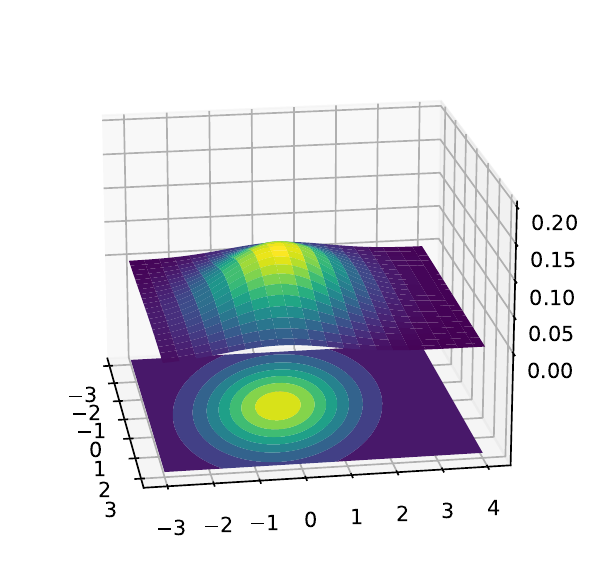} \label{fig:dists_multiGauss_sigma6}}
\centering
\caption{Density and contour plots (\textcolor{mydarkblue}{blue}=low, \textcolor{mydarkyellow}{yellow}=high) for the multivariate Gaussian distribution over the $\mathbb{R}^2$ space for various values of the covariance/scale matrix with a zero-mean vector.  Fig~\ref{fig:dists_multiGauss_sigma1} and \ref{fig:dists_multiGauss_sigma4}: A spherical covariance matrix has a circular shape; 
Fig~\ref{fig:dists_multiGauss_sigma2} and \ref{fig:dists_multiGauss_sigma5}: A diagonal covariance matrix is an \textit{axis aligned} ellipse; 
Fig~\ref{fig:dists_multiGauss_sigma3} and \ref{fig:dists_multiGauss_sigma6}: A full covariance matrix has an  elliptical shape.}
\centering
\label{fig:multi_gaussian_density}
\end{figure}

Similar to the likelihood under univariate Gaussian distribution  (Equation~\eqref{equation:uni_gaussian_likelihood}), especially in deriving the conjugate Bayesian result, 
the likelihood of $N$ random observations $\mathcal{X} = \{\bx_1, \bx_2, \ldots , \bx_N \}$  generated by a multivariate Gaussian with mean vector $\bmu$ and covariance matrix $\bSigma$ is given by 
\begin{equation}\label{equation:multi_gaussian_likelihood}
\begin{aligned}
&\gap p(\mathcal{X} \mid \bmu, \bSigma) =\prod^N_{n=1} \mathcal{N} (\bx_n\mid \bmu, \bSigma) \\
&\overset{(a)}{=} (2\pi)^{-ND/2} \abs{\bSigma}^{-N/2}\exp\left\{-\frac{1}{2} \sum^N_{n=1}(\bx_n - \bmu)^\top \bSigma^{-1}(\bx_n - \bmu)\right\} \\
&\overset{(b)}{=} (2\pi)^{-ND/2} \abs{\bSigma}^{-N/2}\exp\left\{-\frac{1}{2} \tr( \bSigma^{-1}\bS_{\bmu} )  \right\}\\
&\overset{(c)}{=} (2\pi)^{-ND/2} \abs{\bSigma}^{-N/2}\exp\left\{-\frac{N}{2}(\bmu - \overline{\bx})^\top \bSigma^{-1}(\bmu - \overline{\bx})\right\}  \exp\left\{-\frac{1}{2}\tr( \bSigma^{-1}\bS_{\overline{x}} )\right\},
\end{aligned}
\end{equation}
where 
\begin{equation}\label{equation:mvu-sample-covariance}
\bS_{\bmu} \triangleq \sum^N_{n=1}(\bx_n - \bmu)(\bx_n - \bmu)^\top,\quad
\bS_{\overline{x}} \triangleq \sum^N_{n=1}(\bx_n - \overline{\bx})(\bx_n - \overline{\bx})^\top, \quad
\overline{\bx} \triangleq\frac{1}{N}\sum^N_{n=1}\bx_n.
\end{equation}
The matrix $\bS_{\overline{x}}$ is the \textit{matrix of sum of squares} and is also known as the \textit{scatter matrix}.
The equivalence between equation (a) and equation (c)  follows from the following identity (similar reasoning applies to the equivalence between equation (a) and equation (b)):
\begin{align}
\sum^N_{n=1}(\bx_n - \bmu)^\top\bSigma^{-1}(\bx_n - \bmu) = \tr(\bSigma^{-1}\bS_{\overline{x}}) + N \cdot (\overline{\bx} - \bmu)^\top\bSigma^{-1}(\overline{\bx} - \bmu),
\label{equation:multi_gaussian_identity}
\end{align}
where the trace of a square matrix $\bm{A}$ is defined to be the sum of the diagonal elements $a_{ii}$ of $\bm{A}$: $\tr(\bm{A}) = \sum_i a_{ii}$.
\begin{proof}[Proof of Identity~\ref{equation:multi_gaussian_identity}]
There is a ``trick" involving the trace that makes such calculations easy (see also Chapter 3 of \citet{gentle2007matrix}):
\begin{equation}
\bx^\top \bm{A} \bx = \tr(\bx^\top \bm{A} \bx) = \tr(\bx \bx^\top \bm{A}) =  \tr(\bm{A} \bx \bx^\top ),
\end{equation}
where the first equality follows from the fact that $\bx^\top \bm{A} \bx$ is a scalar and the trace of a product is invariant under cyclical permutations of the factors. 

We can then rewrite $\sum^N_{n=1}(\bx_n - \bmu)^\top\bSigma^{-1}(\bx_n - \bmu)$ as 
\begin{equation}
\begin{aligned}
&\, \sum^N_{n=1}(\bx_n - \overline{\bx})^\top\bSigma^{-1}(\bx_n - \overline{\bx}) + \sum^N_{n=1}(\overline{\bx} - \bmu)^\top\bSigma^{-1}(\overline{\bx} - \bmu) \\
&= \tr(\bSigma^{-1} \bS_{\overline{x}}  ) + N \cdot (\overline{\bx} - \bmu)^\top\bSigma^{-1}(\overline{\bx} - \bmu).
\end{aligned}
\end{equation}
This concludes the proof.
\end{proof}

Given fixed mean $\bmu$ and covariance $\bSigma$ parameters, we have 
\begin{equation}\label{equation:multi_gaussian_form_conform}
	\begin{aligned}
		p(\bx\mid \bmu, \bSigma) &=\normal(\bx \mid \bmu, \bSigma)
		\propto \exp\left\{ -\frac{1}{2} \bx^\top\bSigma^{-1}\bx  + \bx^\top \bSigma^{-1}\bmu \right\}.
	\end{aligned}
\end{equation}

\subsection*{Properties of Multivariate Gaussian Distribution}\label{section:multi_gauss}
The entropy of Gaussians (measured in natural units) is discussed in Problem~\ref{problem:entropy_mgau}.
The affine transformation, rotation, independence of linear combinations of multivariate Gaussian distribution also follows the multivariate Gaussian distribution.
\begin{lemma}[Affine transformation of multivariate Gaussian distribution]\label{lemma:affine_mult_gauss}
%If we assume that $\rvx\sim \normal(\bmu, \bSigma)$, then $\bA\rvx+\bb \sim \normal(\bA\bmu+\bb, \bA\bSigma\bA^\top)$ for non-random matrix $\bA$ and vector $\bb$.

Given fixed matrices and vector, $\bA,\bB\in\real^{p\times d}$ and $\bc\in\real^p$, 
let $\rvx\sim \normal(\bmu_x, \bSigma_x)$ and $\rvy\sim \normal(\bmu_y, \bSigma_y)$ be independent variables (of length $d$).
Then, 
$$
\rvz=\bA\rvx+\bB\rvy +\bc \sim \normal(\bA\bmu_x+\bB\bmu_y+\bc, \bA\bSigma_x\bA^\top +\bB\bSigma_y\bB^\top).
$$
Given further $\bd\in\real^{d}$, then $\bd^\top\rvx$ follows from a univariate Gaussian:
$$
\bd^\top\rvx \sim \normal(\bd^\top\bmu_x, \bd^\top\bSigma_x\bd).
$$
\end{lemma}
The result can be proved using the moment generating function of multivariate Gaussian distributions.
The  result also relies on the \textit{sum of independent Gaussians}:
$$
\sum_{i=1}^{n} \rvx_i \sim \normal\big(\sum_{i=1}^{n}\bmu_i, \sum_{i=1}^{n}\bSigma_i\big)
\gap 
\text{if } \rvx_i\sim\normal(\bmu_i,\bSigma_i), \ \forall\, i\in\{1,2,\ldots,n\}.
$$
Moreover, let $\bA=\be_i^\top$ be a unit basis vector, then $\rx_i = \be_i^\top\rvx \sim \normal(\mu_{x,i}, \sigma_{x,ii}^2)$, where $\mu_{x,i}$ represents the $i$-th component of $\bmu_x$ and $\sigma_{x,ii}^2$ denotes the $i$-th diagonal of $\bSigma_x$.
\begin{lemma}[Rotations on  multivariate Gaussian distribution]\label{lemma:rotat_multi_gauss}
	Rotations on the Gaussian distribution do not affect the distribution.
	That is, for any orthogonal matrix $\bQ$ with $\bQ\bQ^\top=\bQ^\top\bQ=\bI$, if $\rvv\sim \normal(\bzero, \sigma^2\bI)$, then $\bQ\rvv\sim \normal(\bzero, \sigma^2\bI)$.
\end{lemma}

\index{Independence}
\begin{lemma}[Independence of linear combinations in Gaussian Distributions]
Suppose $ \rvx \sim \normal(\bmu, \bSigma) \in\real^{p} $ , and let $ \bA\in\real^{m \times p} $, $ \bB\in\real^{d \times p} $ be real matrices. Then,
\begin{equation}
\text{$\bA\rvx$ \text{ is independent of } $\bB\rvx$ $\quad\iff\quad$ $\bA\bSigma \bB^\top = \bzero$.}
\end{equation}
\end{lemma}
The proof again relies on the moment generating function of Gaussian distributions and we shall not provide the details.

\paragrapharrow{``Standardization and decorrelation."}
The distribution $\normal(\bzero, \bI)$ is called the \textit{standard multivariate Gaussian distribution}. 
Given $\rvx\sim\normal(\bmu, \bSigma)$, then the \textit{decorrelation} of $\rvx$ follows that 
\begin{equation}\label{equation:std_mugau_recov}
\rvx\sim\normal(\bmu, \bSigma)
\quad\implies \quad
\rvz =\bSigma^{-1/2}(\rvx-\bmu) \sim \normal(\bzero,\bI).
\end{equation}
This also shows that if $\rvx\sim\normal(\bmu,\bSigma)$, then 
\begin{equation}
\rvx=\bmu+\bSigma^{1/2}\bepsilon,
\gap \text{where }\bepsilon\sim\normal(\bzero,\bI).
\end{equation}

Suppose $\{\bx_i,\bx_2,\ldots,\bx_n\}$ are $n$ random samples of $\normal(\bmu,\bSigma)$ and let $\overline{\bx} = \frac{1}{n} \sum_{i=1}^{n} \bx_i $. Then, it follows that 
\begin{equation}\label{equation:mean_mutigau}
\sqrt{n} (\overline{\bx}-\bmu) \sim \normal(\bzero, \bSigma).
\end{equation}

\paragrapharrow{Partition of Gaussian.}
Let $\rvx \sim \normal(\bmu, \bSigma)$ where $\bmu\in\real^p$. Consider the partition of $\rvx$
$$
\begin{bmatrix}
\rvx_1 \\
\rvx_2
\end{bmatrix}
\sim \normal(\bmu, \bSigma) = \normal
\left(
\begin{bmatrix}
\bmu_1 \\
\bmu_2
\end{bmatrix},
\begin{bmatrix}
\bSigma_{11} & \bSigma_{12} \\
\bSigma_{21} & \bSigma_{22}
\end{bmatrix}
\right).
$$
Then $\rvx_1$ and $\rvx_2$ are independent if and only if $\bSigma_{12} = \bzero$. 
Furthermore, Let $\rvx=[\rx_1, \rx_1, \ldots, \rx_p ]^\top\sim \normal(\bmu, \bSigma)$. Then, 
\begin{equation}\label{equation:iid_mulg_iff7}
\text{the $\rx_i$'s are mutually independent if and only if $\bSigma$ is diagonal.}
\end{equation}
Rigorously, this can be proved as follows:
\begin{proof}[of Equation~\eqref{equation:iid_mulg_iff7}]
Suppose that the $\rx_i$'s are independent. The property below Lemma~\ref{lemma:affine_mult_gauss} yields $\rx_i \sim \normal(\mu_i, \sigma_i^2)$ for some $\sigma_i > 0$. Thus the density of $\rvx$ is
\begin{align*}
p_{\rvx}(\bx)\ &= \prod_{i=1}^{p} p_{\rx_i}(x_i) = \prod_{i=1}^{p} \frac{1}{\sigma_i \sqrt{2\pi}} \exp \left\{ -\frac{1}{2} \frac{(x_i - \mu_i)^2}{\sigma_i^2} \right\} \\
&= \frac{1}{(2\pi)^{p/2} \abs{\diag(\sigma_1^2, \ldots, \sigma_p^2)}^{1/2}} \exp \left\{ -\frac{1}{2} (\bx - \bmu)^\top \diag(\sigma_1^{-2}, \ldots, \sigma_p^{-2}) (\bx - \bmu) \right\}.
\end{align*}
Hence $\rvx \sim \normal(\bmu, \diag(\sigma_1^2, \ldots, \sigma_p^2))$, i.e., the covariance $\bSigma$ is diagonal.

Conversely, assume $\bSigma$ is diagonal, say $\bSigma = \diag(\sigma_1^2, \ldots, \sigma_p^2)$. Then we can reverse the steps of the first part to see that the joint density $p_{\rvx}(\bx)$ can be written as a product of the marginal densities $p_{\rx_i}(x_i)$, thus proving independence.
\end{proof}

It also follows that functions of independent vectors should also be independent. That is, it should be that $g_1(\rvx_1)$ and $g_2(\rvx_2)$ are independent for all $g_1, g_2$. 
Suppose that $\rx_i$ is i.i.d. $\normal(\mu, \sigma^2)$ for $i\in\{1,2,\ldots,p\}$. Then, we can define a vector populated by $\overline{\rx}$ and $\rx_i - \overline{\rx}$:
$$
\begin{bmatrix}
	\overline{\rx} \\
	\rx_1 - \overline{\rx} \\
	\vdots \\
	\rx_p - \overline{\rx}
\end{bmatrix}
=
\begin{bmatrix}
	\frac{1}{p} & \frac{1}{p} & \ldots & \frac{1}{p} \\
	& \bI_p - \frac{1}{p}\bJ_p &
\end{bmatrix}
\begin{bmatrix}
	\rx_1 \\
	\rx_2 \\
	\vdots \\
	\rx_p
\end{bmatrix}
\qquad \text{where} \qquad
\bJ_p =
\begin{bmatrix}
	1 & 1 & \cdots & 1 \\
	1 & \ddots & & 1 \\
	1 & & \ddots & \vdots \\
	1 & 1 & \cdots & 1
\end{bmatrix}\in\real^{p\times p}.
$$
It then follows that 
$$
\begin{bmatrix}
	\overline{\rx} \\
	\rx_1 - \overline{\rx} \\
	\vdots \\
	\rx_p - \overline{\rx}
\end{bmatrix}
%=
%\begin{bmatrix}
%\frac{1}{p} & \frac{1}{p} & \cdots & \frac{1}{p} \\
%& \bI_p - \frac{1}{p}\bJ_p &
%\end{bmatrix}
%\begin{bmatrix}
%\rx_1 \\
%\rx_2 \\
%\vdots \\
%\rx_p
%\end{bmatrix}
\sim
\mathcal{N}
\left(
\begin{bmatrix}
	\mu \\
	0 \\
	\vdots \\
	0
\end{bmatrix},
\sigma^2
\begin{bmatrix}
	\frac{1}{p} & \bzero \\
	\bzero & \bI_p - \frac{1}{p}\bJ_p
\end{bmatrix}
\right).
$$
Therefore, we find that $\overline{\rx}$ is independent of $\rx_1 - \overline{\rx},\rx_2 - \overline{\rx}, \ldots, \rx_p - \overline{\rx}$. 
In many applications, we may construct a random variable:
$$
t \triangleq  \sqrt{p} \frac{(\overline{\rx} - \mu) / \sigma}{\sqrt{\frac{1}{p-1} \sum (y_i - \overline{\rx})^2 / \sigma^2}},
$$
in which case, the numerator and the denominator are independent variables.

\paragrapharrow{Quadratic of Gaussian.} 
The definition of the Chi-squared distribution (Definition~\ref{definition:chisquare_dist}) shows 
$$
\sum_{i=1}^{p} \rx_i^2\sim \chi^2(p), 
\gap \text{if }\rx_i \sim  \normal(0,1).
$$
Therefore, we also have
\begin{equation}
\rvx\sim\normal(\bmu, \bSigma)
\implies 
\rvz =(\rvx-\bmu)^\top \bSigma^{-1}(\rvx-\bmu)\sim \chi^2(p), \gap \text{where }\rvx\in\real^p.
\end{equation}

We may also be interested in the quadratic form of $\rvx^\top\bA\rvx$ where $\bA$ is symmetric. We provide some important results below.
\begin{theoremHigh}[Quadratic of Gaussians]
We have the following results with quadratic forms of Gaussians:
\begin{itemize}
\item 
Given $\rvx\sim\normal(\bzero,\lambda\bI)$ (of length $p$) and symmetric matrix $\bA\in\real^{p\times p}$. Then, it follows that 
$$
\frac{\rvx^\top\bA\rvx}{\lambda} \sim \chisquared_{(n)},
$$
if and only if $\bA$ ($\bA^2=\bA$) is idempotent with rank $n<p$.
\item Given $\rvx\sim\normal(\bzero,\bSigma)$ (of length $p$) and symmetric matrix $\bA\in\real^{p\times p}$. Then, it follows that 
$$
\rvx^\top\bA\rvx \sim \chisquared_{(n)},
$$
if and only if $\bA\bSigma$ is idempotent with rank $n<p$.
\end{itemize}
\end{theoremHigh}

%Every marginal distribution of a multivariate Gaussian distribution is itself a multivariate Gaussian distribution.
%Suppose $\rvx\sim\normal(\bmu,\bSigma)$ with $\rvx\in\real^p$ and 
%$$
%\rvx=
%\begin{bmatrix}
%\rvx_1\\
%\rvx_2 
%\end{bmatrix},
%\gap 
%\bmu=
%\begin{bmatrix}
%	\bmu_1\\
%	\bmu_2 
%\end{bmatrix},
%\gap 
%\bSigma = 
%\begin{bmatrix}
%\bA & \bC^\top\\
%\bC & \bB
%\end{bmatrix},
%$$
%where $\rvx_1,\bmu_1\in\real^k$ and $\bA\in\real^{k\times k}$. 
%Then, it follows that 
%$
%\rvx_1\sim\normal(\bmu_1,\bA).
%$
%
%The conditional distribution $\rvx_1\mid \rvx_2$ also follows a multivariate Gaussian distribution:
%$$
%\rvx_1\mid \rvx_2 \sim
%\normal(\bmu_1+\bC^\top\bB^{-1}(\rvx_2-\bmu_2), \bA-\bC^\top\bB^{-1}\bC).
%$$

\paragrapharrow{Marginal and conditional distributions.}
Let $\rvx$ and $\rvy$ be jointly Gaussian random vectors with 
$$
\rvz=
\begin{bmatrix}
\rvx\\
\rvy 
\end{bmatrix}
\sim 
\normal\left(
\begin{bmatrix}
\bmu_x\\
\bmu_y 
\end{bmatrix}
,
\begin{bmatrix}
\bA & \bC\\
\bC^\top & \bB
\end{bmatrix}
\right)=
\normal\left(
\begin{bmatrix}
\bmu_x\\
\bmu_y 
\end{bmatrix}
,
\begin{bmatrix}
\widetildebA & \widetildebC\\
\widetildebC^\top & \widetildebB
\end{bmatrix}^{-1}
\right).~\footnote{
Given nonsingular $\bM$ and its inverse $\bM^{-1}$; and suppose appropriate sizes for the following partitions \citep{williams2006gaussian}:
$$
\bM=
\begin{bmatrix}
\bA & \bB\\
\bC & \bD
\end{bmatrix},
\gap 
\bM^{-1}=
\begin{bmatrix}
\widetildebA & \widetildebB\\
\widetildebC & \widetildebD
\end{bmatrix}.
$$
We have 
\begin{equation}\label{equation:mt_inv}
\begin{aligned}
&\widetildebA=\bA^{-1}+\bA^{-1}\bB\widetildebD\bC\bA^{-1}&=& (\bA-\bB\bD^{-1}\bC)^{-1} , \\
&\widetildebB=-\bA^{-1}\bB\widetildebD&=& -\widetildebA\bB\bD^{-1},\\
&\widetildebC=-\widetildebD\bC\bA^{-1}&=& -\bD^{-1}\bC\widetildebA, \\
& \widetildebD=(\bD-\bC\bA^{-1}\bB)^{-1}&=& \bD^{-1}+\bD^{-1}\bC\widetildebA\bB\bD^{-1},
\end{aligned}
\end{equation}
}
$$
where $\rvx$ and $\rvy$ are \textit{independent} if and only if $\Cov[\rvx,\rvy]=\bC=\bzero$.
Then every marginal distribution of a multivariate Gaussian distribution is itself a multivariate Gaussian distribution, and the conditional distribution $\rvx\mid \rvy$ also follows a multivariate Gaussian distribution:
\begin{equation}
\begin{aligned}
\rvx
\sim\normal(\bmu_x,\bA),
\gap 
\rvx\mid \rvy=\by 
&\sim \normal(\bmu_x+\bC\bB^{-1}(\by-\bmu_y), \bA-\bC\bB^{-1}\bC^\top)\\
&=\normal(\bmu_x-\widetildebA^{-1}\widetildebC(\by-\bmu_y), \widetildebA^{-1});\\
\rvy
\sim\normal(\bmu_y,\bB),
\gap 
\rvy\mid \rvx=\bx
&\sim \normal(\bmu_y+\bC^\top\bA^{-1}(\bx-\bmu_x), \bB-\bC^\top\bA^{-1}\bC)\\
&=\normal(\bmu_y-\widetildebB^{-1}\widetildebC^\top(\bx-\bmu_x), \widetildebB^{-1}).\\
\end{aligned}
\end{equation}
\begin{proof}[Short proof]
Suppose $\rvx^\prime =\rvx-\bC\bB^{-1}$. Then 
$$
\rvz^\prime = 
\begin{bmatrix}
\rvx^\prime \\
\rvy 
\end{bmatrix}
=
\begin{bmatrix}
\bI & -\bC\bB^{-1}\\
\bzero & \bI 
\end{bmatrix}
\rvz.
$$
Using Lemma~\ref{lemma:affine_mult_gauss}, we can show that $\rvx^\prime$ and $\rvy$ are independent.
Then, the conditional distribution of $\rvx \mid \rvy$ can be obtained by $\rvx=\rvx^\prime + \bC\bB^{-1}\rvy$  and following the distribution law.
The second part can be proved similarly.
\end{proof}

This relationship is useful for finding Gaussian-related distributions. See the exercise below.
\begin{exercise}[Affine dependence of Gaussian variables]
Suppose random vectors $\rvx\sim\normal(\bmu, \bSigma)$ and $\rvy\mid \rvx=\bx\sim\normal(\bA\bx+\bb, \bM)$.
Note $\rvy$ is not simply the affine transformation $\bA\rvx+\bb$, but it follows that $\rvy=\bA\bx+\bb+\bepsilon$ where $\bepsilon\sim \normal(\bzero, \bM)$.
Show that 
$$
\rvy\sim \normal(\bA\bmu+\bb, \bM+\bA\bSigma\bA^\top),
\gap 
\rvx \mid \rvy \sim \normal\big(\bL\big\{\bA^\top\bM^{-1}(\rvy-\bb)+\bSigma^{-1}\bmu  \big\}, \bL\big), 
$$
where $\bL=(\bSigma^{-1}+\bA^\top\bM^{-1}\bA)^{-1}$.
\textit{Hint: compute the cross-covariance of $\rvx$ and $\rvy$ by $\Cov[\rvx,\rvy]=\Exp[(\rvx-\bmu_x)(\rvy-\bmu_y)^\top]=\bSigma\bA^\top$ where $\bmu_x=\bmu$ and $\bmu_y=\Exp[\rvy]=\bA\bmu+\bb$, and use Woodbury matrix identity: $(\bA+\bB\bD\bC)^{-1} = \bA^{-1} - \bA^{-1} \bB(\bD^{-1} + \bC\bA^{-1}\bB)^{-1}\bC\bA^{-1}$ for appropriate matrices $\bA,\bB,\bC$, and $\bD$; see, for example,  \citet{lu2021numerical}}.
\end{exercise}

\paragrapharrow{Product of Gaussians.}
The product of two Gaussians also follows a Gaussian distribution (although no longer normalized) \citep{ahrendt2005multivariate}. 
Given two Gaussians $\normal(\bmu_a, \bSigma_a)$ and $\normal(\bmu_b, \bSigma_b)$ (both of length $p$), it follows that 
\begin{equation}
\normal(\bmu_a, \bSigma_a)\cdot \normal(\bmu_b, \bSigma_b) \propto z_c\normal(\bmu_c,\bSigma_c),
\end{equation}
where 
$$
\bSigma_c=(\bSigma_a^{-1}+\bSigma_b^{-1})^{-1},
\gap
\text{and}
\gap 
\bmu_c = \bSigma_c(\bSigma_a^{-1}\bmu_a + \bSigma_b^{-1}\bmu_b).
$$
That is, the resulting precision matrix is the sum of precision matrices of the two components.
And $z_c$ is a normalization constant
$$
\begin{aligned}
	z_c=\abs{2\pi \bSigma_a\bSigma_b\bSigma_c^{-1}}^{-\frac{1}{2}} \exp\left\{-\frac{1}{2}(\bmu_a-\bmu_b)^\top \bSigma_a^{-1}\bSigma_c\bSigma_b^{-1}(\bmu_a-\bmu_b)\right\}.
\end{aligned}
$$

\subsection*{Multivariate Student's $t$ Distribution}
The multivariate Student's $t$-distribution is a continuous probability distribution over multiple variables that generalizes the Gaussian distribution to allow for heavier tails, 
i.e., the probability of extreme values is higher than that in a Gaussian distribution.
The multivariate Student's $t$ distribution (simply called Student's $t$ distribution when it's clear from the context) will be often used in the posterior predictive distribution of multivariate Gaussian parameters. We rigorously define the distribution as follows.

\index{Multivariate Student's $t$ distribution}
\begin{definition}[Multivariate Student's $t$ distribution]\label{definition:multivariate-stu-t}
A random vector $\rvx\in\real^D$ is said to follow the \textit{multivariate Student's $t$ distribution} with parameters $\bmu\in\real^D$, $\bSigma\in\real^{D\times D}$, and $\nu$, denoted by $\rvx \sim \tau( \bmu, \bSigma, \nu)$, if
$$
\begin{aligned}
f(\bx; \bmu, \bSigma, \nu)&= \frac{\Gamma(\nu/2 + D/2)}{\Gamma(\nu/2)} \frac{\abs{\bSigma}^{-1/2}}{\nu^{D/2} \pi^{D/2}} \times \left[ 1+ \frac{1}{\nu} (\bx-\bmu)^\top \bSigma^{-1} (\bx-\bmu)  \right]^{-(\frac{\nu+D}{2})}\\
&= \frac{\Gamma(\nu/2 + D/2)}{\Gamma(\nu/2)} |\pi\bV|^{-1/2} \times \left[ 1+ \frac{1}{\nu} (\bx-\bmu)^\top \bV^{-1} (\bx-\bmu)  \right]^{-(\frac{\nu+D}{2})},
\end{aligned}
$$
where $\bSigma$ is called the \textit{scale matrix} and $\bV=\nu\bSigma$, and $\nu$ is the \textit{degree of freedom}. This distribution has fatter tails than a Gaussian one. The smaller the $\nu$ is, the fatter the tails. As $\nu \rightarrow \infty$, the distribution converges towards a multivariate Gaussian.
The mean, mode, and covariance of the multivariate Student's $t$ distribution are given by 
\begin{equation*}
\begin{aligned}
\Exp [\bx] &= \bmu, \qquad 
\mathrm{Mode}[\bx] = \bmu, \qquad\text{and}\qquad
\Cov [\bx] = \frac{\nu}{\nu-2}\bSigma. 
\end{aligned}
\end{equation*}
Note that the $\bSigma$ is called the scale matrix since it is not exactly the covariance matrix as that in a multivariate Gaussian distribution. 

Specifically, When $D=1$, it follows that
\begin{equation}\label{equation:uni-stu-nonzero}
\begin{aligned}
\tau(x\mid \mu, \sigma^2, \nu)&= \frac{\Gamma(\frac{\nu+1}{2})}{\Gamma(\frac{\nu}{2})} \frac{1}{\sigma\sqrt{\nu\pi}} \times \left[ 1+ \frac{(x-\mu)^2}{\nu \sigma^2}   \right]^{-(\frac{\nu+1}{2})}.
%&= \frac{\Gamma(\nu/2 + D/2)}{\Gamma(\nu/2)} |\pi\bV|^{-1/2} \times \left[ 1+ \frac{1}{\nu} (\bx-\bmu)^\top \bV^{-1} (\bx-\bmu)  \right]^{-(\frac{\nu+1}{2})},
\end{aligned}
\end{equation}
When $D=1, \bmu=0, \bSigma=1$, then the p.d.f. defines the \textit{univariate $t$ distribution}.
\begin{equation*}
\begin{aligned}
\tau(x\mid \nu)&= \frac{\Gamma(\frac{\nu+1}{2})}{\Gamma(\frac{\nu}{2})} \frac{1}{\sqrt{\nu\pi}} \times \left[ 1+ \frac{x^2}{\nu }   \right]^{-(\frac{\nu+1}{2})}.
\end{aligned}
\end{equation*}
\end{definition}
Figure~\ref{fig:studentt_densitys-1} compares the Gaussian and the Student's $t$ distribution for various values such that when $\nu\rightarrow \infty$, the difference between the densities is approaching zero. Given the same parameters in the densities, the Student's $t$ in general has longer ``tails" than a Gaussian, which can be seen from the comparison between Figure~\ref{fig:gauss-diagonal} and Figure~\ref{fig:student-1}. 
This provides the Student's $t$ distribution an important property known as \textbf{robustness}, which means that it is much less sensitive than the Gaussian in the presence of  outliers \citep{bishop2006pattern, murphy2012machine}.

\begin{figure}[h]
	\subfigure[Gaussian, $\bSigma =\begin{bmatrix}
		1&0\\
		0&1
	\end{bmatrix}. $ ]{\includegraphics[width=0.31
		\textwidth]{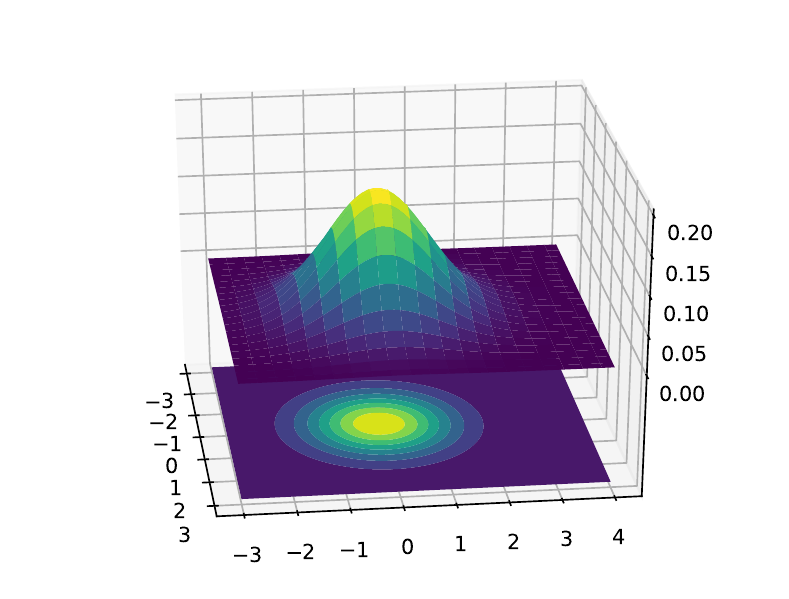} \label{fig:gauss-diagonal}}
	\subfigure[Gaussian, $\bSigma =\begin{bmatrix}
		1&0\\
		0&3
	\end{bmatrix}.$]{\includegraphics[width=0.31
		\textwidth]{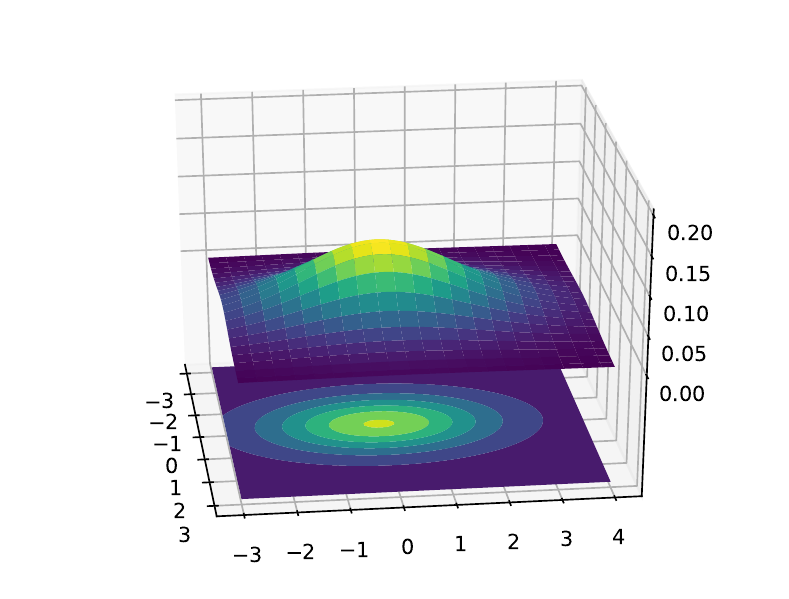} \label{fig:gauss-spherical}}
	\subfigure[Gaussian, $\bSigma =\begin{bmatrix}
		1&\textendash0.5\\
		\textendash0.5&1.5
	\end{bmatrix}.$]{\includegraphics[width=0.31 
		\textwidth]{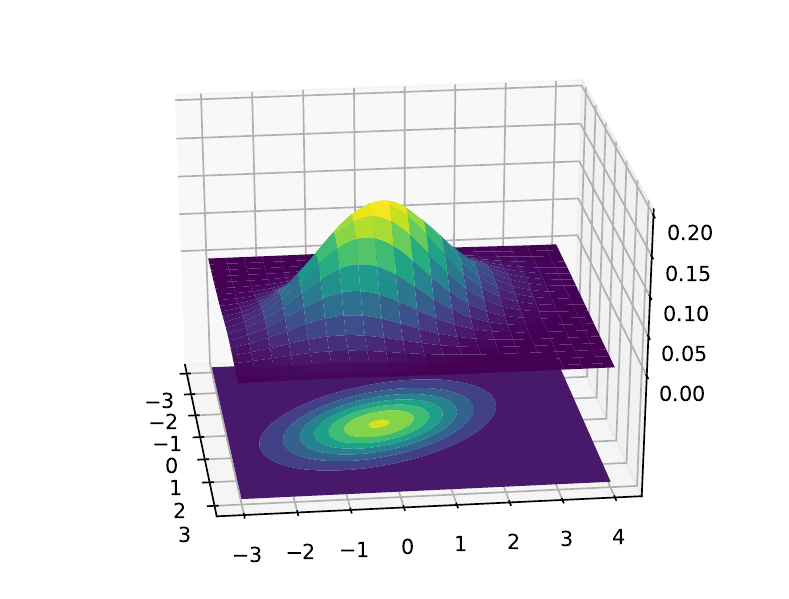} \label{fig:gauss-full}}
	\subfigure[Student $t$, $\bSigma =\begin{bmatrix}
		1&0\\
		0&1
	\end{bmatrix}, \nu=1. $ ]{\includegraphics[width=0.31
		\textwidth]{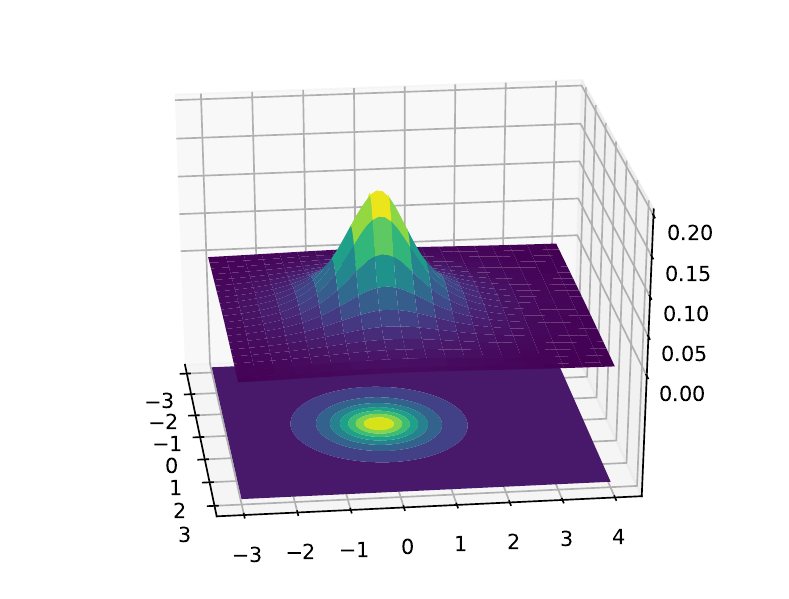} \label{fig:student-1}}
	\subfigure[Student $t$, $\bSigma =\begin{bmatrix}
		1&0\\
		0&1
	\end{bmatrix}, \nu=3. $ ]{\includegraphics[width=0.31
		\textwidth]{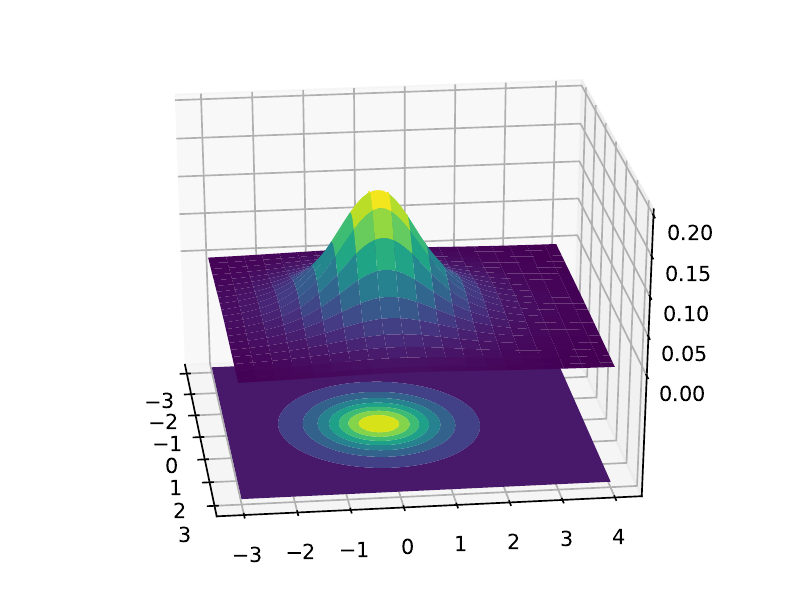} \label{fig:student-3}}
	\subfigure[Stu $t$, $\bSigma =\begin{bmatrix}
		1&0\\
		0&1
	\end{bmatrix}, \nu=200. $ ]{\includegraphics[width=0.31
		\textwidth]{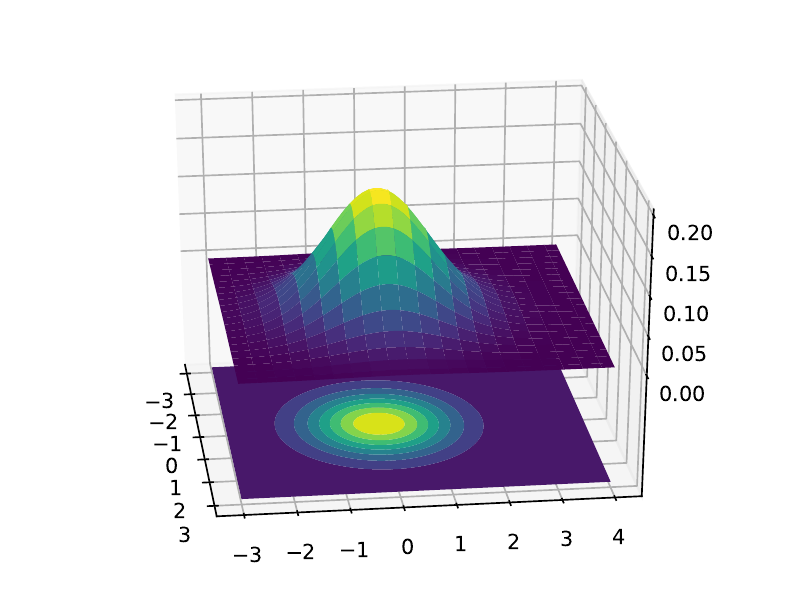} \label{fig:student200}}
	\subfigure[Diff between (a) and (d)]{\includegraphics[width=0.31
		\textwidth]{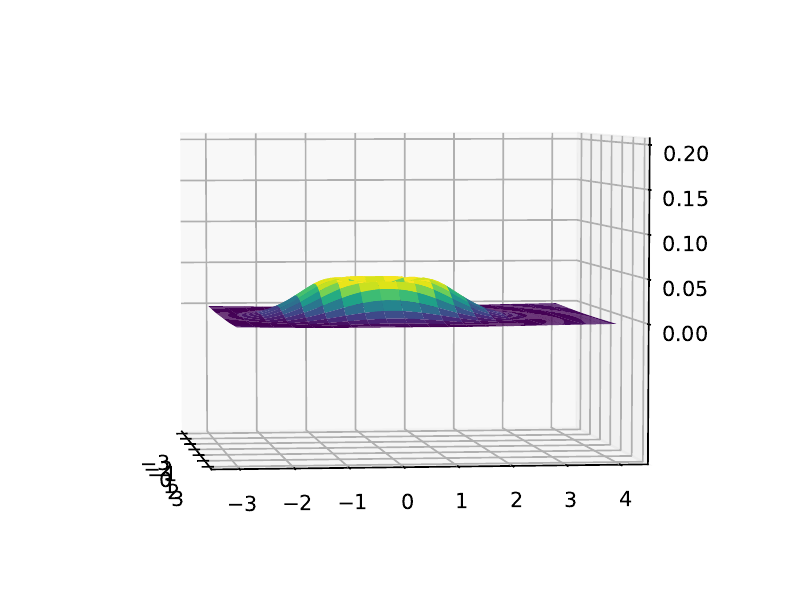} \label{fig:gauss-stu-diff1}}
	\subfigure[Diff between (a) and (e)]{\includegraphics[width=0.31
		\textwidth]{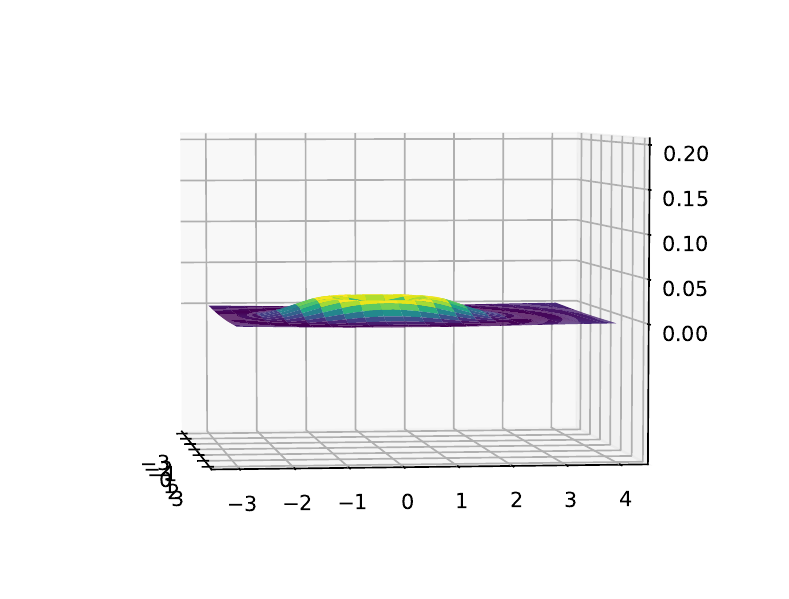} \label{fig:gauss-stu-diff3}}
	\subfigure[Diff between (a) and (f)]{\includegraphics[width=0.31
		\textwidth]{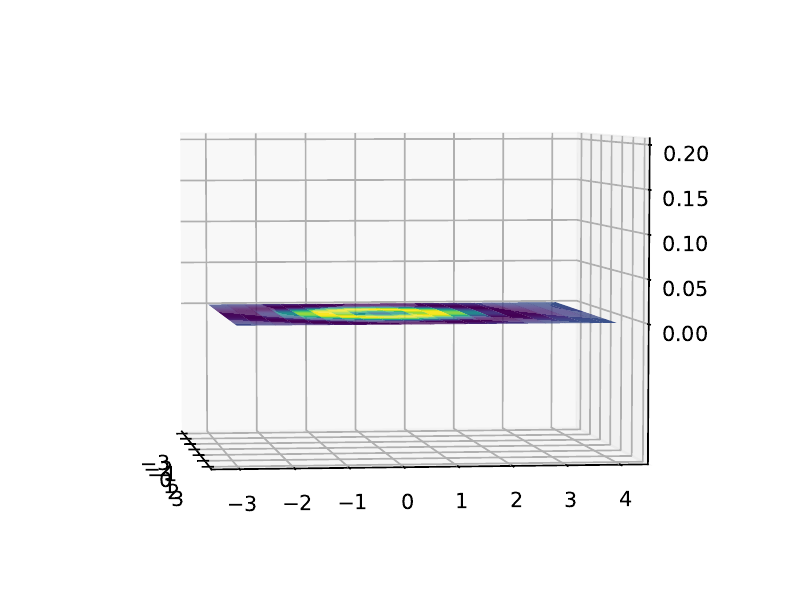} \label{fig:gauss-stu-diff200}}
	\centering
	\caption{Density and contour plots (\textcolor{mydarkblue}{blue}=low, \textcolor{mydarkyellow}{yellow}=high) for the multivariate Gaussian distribution and multivariate Student's $t$ distribution over the $\mathbb{R}^2$ space for various values of the covariance/scale matrix with zero-mean vector.  
		Fig~\ref{fig:gauss-diagonal}: A spherical covariance matrix has a circular shape; 
		Fig~\ref{fig:gauss-spherical}: A diagonal covariance matrix is an \textit{axis aligned} ellipse; 
		Fig~\ref{fig:gauss-full}: A full covariance matrix has a elliptical shape; \\
		Fig~\ref{fig:student-1} to Fig~\ref{fig:student200} for the Student's $t$ distribution with the same scale matrix and increasing $\nu$ such that the difference between (a) and (f) in Fig~\ref{fig:gauss-stu-diff200} is approaching  zero.}\centering
	\label{fig:studentt_densitys-1}
\end{figure}

A Student's $t$ distribution can be written as a \textit{Gaussian scale mixture}
\begin{equation}\label{equation:gauss-scale-mixture}
	\tau(\bx\mid \bmu, \bSigma, \nu) 
	= \int_0^{\infty} \normal(\bx \mid \bmu, \bSigma/z)\cdot \gammadist\big(z\mid \frac{\nu}{2}, \frac{\nu}{2}\big)
	dz.
\end{equation}
This can be thought of as an ``infinite'' mixture of Gaussians, each with a slightly different covariance matrix. In other words, a Student's $t$ distribution is obtained by adding up an
infinite number of Gaussian distributions having the same mean vector but different covariance matrices. 
From this Gaussian scale mixture view, when $\nu \rightarrow \infty$, the Gamma distribution becomes a degenerate random variable with all the nonzero mass at the point unity such that the multivariate Student's $t$ distribution converges to a multivariate Gaussian distribution.

\paragrapharrow{Affine transformations of Student's $t$.}
Similar to the multivariate Gaussian distribution, the affine transformation of a Student's $t$ also follows another Student's $t$. 
Suppose $\rvx\sim \studentt(\bmu, \bSigma, \nu)$ (of length $D$) and given a  fixed matrix $\bA\in\real^{P\times D}$ and a fixed vector $\bb\in\real^P$. Then it follows that 
\begin{equation}
	\bA\rvx \sim \studentt(\bA\bmu+\bb, \bA\bSigma\bA^\top, \nu).
\end{equation}
Therefore, we can sample $\rvx\sim \studentt(\bmu, \bSigma, \nu)$ by sampling $\rvy\sim\studentt(\bzero, \bI, \nu)$ and letting $\rvx=\bmu+\bL\rvy$, where $\bSigma=\bL\bL^\top$ is the Cholesky decomposition of $\bSigma$ (Theorem~\ref{theorem:cholesky-factor-exist}).

\paragrapharrow{Marginal and conditional distributions of Student's $t$.}
Similar to the multivariate Gaussian distribution, let $\rvx$ and $\rvy$ be jointly Student's $t$ random vectors with 
$$
\rvz=
\begin{bmatrix}
	\rvx\\
	\rvy 
\end{bmatrix}
\sim 
\studentt\left(
\begin{bmatrix}
	\bmu_x\\
	\bmu_y 
\end{bmatrix}
,
\begin{bmatrix}
	\bA & \bC\\
	\bC^\top & \bB
\end{bmatrix} , \nu
\right)=
\studentt\left(
\begin{bmatrix}
	\bmu_x\\
	\bmu_y 
\end{bmatrix}
,
\begin{bmatrix}
	\widetildebA & \widetildebC\\
	\widetildebC^\top & \widetildebB
\end{bmatrix}^{-1}, \nu
\right),
$$
where $\rvx\in\real^{d_x}$ and $\rvy\in\real^{d_y}$.
Then every marginal distribution of a Student's $t$ distribution is itself a Student's $t$ distribution, and the conditional distribution $\rvx\mid \rvy$ also follows a Student's $t$ distribution:
\begin{equation}
	\begin{aligned}
		\rvx
		\sim\studentt(\bmu_x,\bA, \nu),
		\gap 
		\rvx\mid \rvy=\by 
		&\sim \studentt(\bmu_x+\bC\bB^{-1}(\by-\bmu_y), m_x(\bA-\bC\bB^{-1}\bC^\top), \nu+d_x)\\
		&=\studentt(\bmu_x-\widetildebA^{-1}\widetildebC(\by-\bmu_y), m_x\widetildebA^{-1}, \nu+d_x);\\
		\rvy
		\sim\studentt(\bmu_y,\bB, \nu),
		\gap 
		\rvy\mid \rvx=\bx
		&\sim \studentt(\bmu_y+\bC^\top\bA^{-1}(\bx-\bmu_x), m_y(\bB-\bC^\top\bA^{-1}\bC), \nu+d_y)\\
		&=\studentt(\bmu_y-\widetildebB^{-1}\widetildebC^\top(\bx-\bmu_x), m_y\widetildebB^{-1}, \nu+d_y),\\
	\end{aligned}
\end{equation}
where 
$$
\begin{aligned}
	m_x &= \frac{1}{\nu+d_y}\left[\nu+ (\by-\bmu_y)^\top \bB^{-1} (\by-\bmu_y)\right];\\
	m_y &= \frac{1}{\nu+d_x}\left[\nu+ (\bx-\bmu_x)^\top \bA^{-1} (\bx-\bmu_x)\right].
\end{aligned}
$$

%We realize the fact that two independent Gamma random variables such that $X\sim \gammadist(a, \gamma)$, $Y \sim \gammadist(b, \gamma)$ will imply $X+Y \sim \gammadist(a+b, \gamma)$. This states that $\nu z \sim \gammadist(\frac{\nu^2}{2}, \frac{\nu}{2})$

\subsection*{Wishart Distribution and Variants}
A generalization to the inverse-Gamma distribution (Definition~\ref{definition:inverse_gamma_distribution}) is the \textit{inverse-Wishart} distribution, serving as a conjugate prior for the full covariance matrix of a multivariate Gaussian distribution. 
That is, the inverse-Wishart distribution is a probability distribution of random positive definite matrices that can be used to model random covariance matrices.

Before delving into the topic of the inverse-Wishart distribution, it's important to note that it originates from the Wishart distribution, a multidimensional generalization of the Gamma distribution. As stated by \citet{anderson1962introduction} in 1962,  ``\textit{The Wishart distribution ranks next to the (multivariate) normal distribution in order of importance and usefulness in multivariate statistics.}"

\index{Wishart distribution}
\begin{definition}[Wishart distribution]\label{definition:wishart_dist}
A random symmetric positive definite matrix $\bLambda\in \real^{D\times D}$ is said to follow the \textit{Wishart distribution} with parameter $\bM\in\real^{D\times D}$ and $\nu$, denoted by $\bLambda \sim \wishartdist(\bM, \nu)$, if
$$
\small
\begin{aligned}
f(\bLambda;\textcolor{black}{\bM}, \nu)
&= \abs{\bLambda}^{\textcolor{black}{\frac{\nu-D-1}{2}}} \exp\left\{-\frac{1}{2}\tr(\textcolor{black}{\bLambda} \textcolor{black}{\bM^{-1}})\right\}
\left[2^{\frac{\nu D}{2}}  \pi^{D(D-1)/4}  \textcolor{black}{|\bM|^{\nu/2}  } \prod_{d=1}^D\Gamma\big(\frac{\nu+1-d}{2}\big) \right]^{-1},
\footnote{In some texts, the density function is defined using the generalized Gamma function: $\Gamma_d(x)=\pi^{d(d-1)/4}\prod_{i=1}^{d}\Gamma(\frac{2x+1-i}{2})$, such that 
$$
f(\bLambda;\textcolor{black}{\bM}, \nu)\\
= \abs{\bLambda}^{\textcolor{black}{\frac{\nu-D-1}{2}}} \exp\left\{-\frac{1}{2}\tr(\textcolor{black}{\bLambda} \textcolor{black}{\bM^{-1}})\right\}
\left[2^{\frac{\nu D}{2}}    \textcolor{black}{|\bM|^{\nu/2}  } \Gamma_D(\nu/2) \right]^{-1}.
$$
}
\end{aligned}
$$
where $\nu \geq D$ and $\bM$ is a $D\times D$ symmetric positive definite matrix, and $\abs{\bLambda} = \det(\bLambda)$ is the determinant of matrix $\bLambda$.
The parameter $\nu$ is called the \textit{number of degrees of freedom}, and $\bM$ is called the \textit{scale matrix}.
The mean and variance of the Wishart distribution are given by 
$$
\begin{aligned}
\Exp [\bLambda] &= \nu \bM \qquad \text{and}\qquad 
\Var[\lambda_{ij}] = \nu (m_{ij}^2 + m_{ii}m_{jj}),
\label{equation:wishart_expectation}
\end{aligned}
$$
where $m_{ij}$ is the ($i,j$)-th element of $\bM$.
It can be shown that when $\nu\rightarrow \infty$, then $\bLambda/\nu$ converges in probability to $\bM$ (using law of large numbers and the Cram\'er-Wold device).

When $D=1$ and $\bM=1$, the Wishart distribution reduces to the Chi-squared distribution (Definition~\ref{definition:chisquare_dist}) such that:
$$
\wishartdist(x\mid  1, \nu) = \chisquared(x\mid \nu).
$$
\end{definition}

An interpretation of the Wishart distribution is as follows. Suppose we independently sample vectors $\bz_1, \bz_2, \ldots, \bz_{\nu}\in\real^D$ from $\normal(\bzero, \bM)$. The sum of squares matrix of the collection of multivariate vectors is given by 
$$
\sum_{i=1}^{\nu} \bz_i\bz_i^\top = \bZ^\top\bZ,
$$
where $\bZ$ is the $\nu \times D$ matrix with $i$-th row being $\bz_i$. It is evident that $\bZ^\top\bZ$ is positive semidefinite (PSD) and symmetric. If $\nu >D$ and the $\bz_i$'s are linearly independent, then $\bZ^\top \bZ$ will be positive definite (PD) and symmetric. 
In other words, $\bZ\bx=\bzero$ only happens when $\bx=\bzero$. We can repeat over and over again, generating matrices $\bZ_1^\top\bZ_1, \bZ_2^\top\bZ_2, \ldots, \bZ_l^\top\bZ_l$. The population distribution of these matrices follows a Wishart distribution with parameters $(\bM, \nu)$. By definition, 
$$
\begin{aligned}
\bLambda&=\bZ^\top\bZ = \sum_{i=1}^{\nu} \bz_i\bz_i^\top 
\qquad\implies\qquad
\Exp[\bLambda]=\Exp[\bZ^\top\bZ] = \Exp\left[\sum_{i=1}^{\nu} \bz_i\bz_i^\top\right] = \nu \Exp[\bz_i\bz_i^\top] = \nu\bM. \\
\end{aligned}
$$

When $D=1$, this reduces to the case that if $z$ is drawn from a zero-mean univariate normal random variable, then $z^2$ is drawn from a Gamma random variable. To be specific, 
$$
\mathrm{suppose } \qquad  z \sim \normal(0, a) , \qquad \mathrm{then } \qquad z^2\sim \gammadist(a/2, 1/2).
$$

\begin{remark}[Properties of Wishart distribution]
We present several properties of the Wishart distribution without providing their proofs:
\begin{itemize}
\item \textit{``Decorrelation."} Suppose $\bLambda\sim\wishartdist(\bM, \nu)$ with $\bLambda\in\real^{D\times D}$. 
Then, it follows that $\bM^{-1/2}\bLambda\bM^{-1/2}\sim \wishartdist(\nu, \bI_D)$.
\item \textit{Quadratic transformation.} Suppose $\bLambda\sim\wishartdist(\bM, \nu)$ with $\bLambda\in\real^{D\times D}$ and $\bA\in\real^{P\times D}$. 
Then, it follows that  $\bA\bLambda\bA^\top\sim \wishartdist(\bA\bM\bA^\top, \nu)$.
\item Suppose $\bLambda\sim\wishartdist(\bM, \nu)$ with $\bLambda\in\real^{D\times D}$, $\ba\in\real^D$, and $\nu>D-1$. 
Then, it follows that $\frac{\ba^\top\bM^{-1}\ba}{\ba^\top\bLambda^{-1}\ba}\sim \chisquared_{(\nu-D-1)}$.
\item Suppose $\bLambda\sim\wishartdist(\bM, \nu)$ with $\bLambda\in\real^{D\times D}$ and $\ba\in\real^D$. 
Then, it follows that $\frac{\ba^\top\bLambda\ba}{\ba^\top\bM\ba}\sim \chisquared_{(\nu)}$.
\item \textit{\text{Sum of independent Wisharts}.} Given independent random matrices $\bLambda_i\sim\wishartdist(\bM, \nu_i)$ with $\nu=\sum_{i}\nu_i$. 
Then, it follows that $\sum_{i}\bLambda_i\sim \wishartdist(\bM, \nu)$.
\item \textit{\text{Sum of independent Wisharts}.} Similarly, given independent random matrices $\bLambda\sim\wishartdist(\bM, \nu)$ and $\bLambda_1\sim\wishartdist(\bM, \nu_1)$. 
Then, it follows that $\bLambda_2=\bLambda-\bLambda_1\sim \wishartdist(\bM,\nu-\nu_1)$.
\item \textit{``Standardization".} Suppose $\{\bx_i,\bx_2,\ldots,\bx_n\}$ are random samples of $\normal(\bmu,\bSigma)$, let $\overline{\bx} = \frac{1}{n} \sum_{i=1}^{n} \bx_i $ and $\bS=\frac{1}{n-1}\sum_{i=1}^{n}(\bx_i-\overline{\bx})(\bx_i-\overline{\bx})^\top$.
Then, it follows that $(n-1)\bS\sim \wishartdist(\bSigma,n-1)$. And it can be shown that $\overline{\bx}$ and $\bS$ are independent (the distribution of $\overline{\bx}$ is shown in \eqref{equation:mean_mutigau}). 
\end{itemize}
\end{remark}

Just like the relationship between the inverse-Gamma distribution and the Gamma distribution that if $x \sim \gammadist(r, \lambda)$, then $y=\frac{1}{x} \sim \inversegammadist(r, \lambda)$. There is a similar connection between the   inverse-Wishart distribution and the Wishart distribution.

Since we often use the inverse-Wishart (IW) distribution as a prior distribution for a covariance matrix, it is often useful to replace $\bM$ in the Wishart distribution with $\bS=\bM^{-1}$. This results in that
a random $D\times D$ symmetric positive definite matrix $\bSigma$ follows an inverse-Wishart $\inversewishart(\bSigma\mid \bS, \nu)$ distribution if $\bSigma^{-1}=\bLambda$ follows a Wishart $\mathrm{Wi}(\bLambda\mid \bM, \nu)$ distribution. 
%\footnote{However, if we do not replace $\bM$ by $\bS$, the relationship will be: A random $D\times D$ symmetric positive definite matrix $\bSigma$ has an $\inversewishart(\bSigma| \textcolor{mydarkblue}{\bM}, \nu)$ distribution if $\bSigma^{-1}$ has a Wishart $\mathrm{Wi}(\bLambda| \bM, \nu)$ distribution.}
\begin{definition}[Inverse-Wishart distribution]\label{definition:multi_inverse_wishart}
A random symmetric positive definite matrix $\bSigma\in \real^{D\times D}$ is said to follow the \textit{inverse-Wishart distribution} with parameters $\bS\in\real^{D\times D}$ and $\nu$, denoted by $\bSigma\sim \inversewishart(\bS, \nu)$, if
$$
\small
\begin{aligned}
f(\bSigma; \textcolor{winestain}{\bS}, \nu)
= \abs{\bSigma}^{\textcolor{mydarkblue}{-\frac{\nu+D+1}{2}}} \exp\left\{-\frac{1}{2}\tr(\textcolor{mydarkblue}{\bSigma^{-1}} \textcolor{winestain}{\bS})\right\}
\times \left[2^{\frac{\nu D}{2}}  \pi^{D(D-1)/4}  \textcolor{winestain}{|\bS|^{-\nu/2}}   \prod_{d=1}^D\Gamma\big(\frac{\nu+1-d}{2}\big) \right]^{-1},
\end{aligned}
$$
where $\nu \geq D$, $\bS$ is a $D\times D$ symmetric positive definite matrix, and $\abs{\bSigma} = \det(\bSigma)$.
The parameter $\nu$ is called the \textit{number of degrees of freedom}, and $\bS$ is called the \textit{scale matrix}.
The mean and mode of the inverse-Wishart distribution are given by 
$$
\begin{aligned}
\Exp [\bSigma ^{-1}] &= \nu \bS^{-1}=\nu \bM, \qquad 
\Exp [\bSigma] = \frac{1}{\nu - D - 1} \bS, \qquad \text{and}\qquad
\mathrm{Mode}[\bSigma] = \frac{1}{\nu + D + 1} \bS.
\label{equation:iw_expectation}
\end{aligned}
$$
Note that, sometimes, we replace $\bS$ by $\bM=\bS^{-1}$ such that $\Exp [\bSigma ^{-1}] = \nu \bM$, which does not involve the inverse of the matrix.

When $D=1$, the inverse-Wishart distribution reduces to the inverse-Gamma such that $\frac{\nu}{2} = r$ and $\frac{S}{2}=\lambda$ (see Definition~\ref{definition:inverse_gamma_distribution}):
$$
\inversewishart(y\mid S, \nu) = \inversegammadist(y\mid r, \lambda).
$$
\end{definition}

Note that the Wishart density is not simply the inverse-Wishart density with
$\bSigma$ replaced by $\bLambda = \bSigma^{-1}$. There is an additional factor of $\abs{\bSigma}^{-(D+1)}$. See  Theorem 7.7.1 in \citet{anderson1962introduction} that the Jacobian of the transformation $\bLambda = \bSigma^{-1}$ is $\abs{\bSigma}^{-(D+1)}$. Substitution of $\bSigma^{-1}$ in the definition of the Wishart distribution and multiplying by $\abs{\bSigma}^{-(D+1)}$ can yield the inverse-Wishart distribution.
\footnote{Which is from the Jacobian in the change-of-variables formula. A short proof is provided here. Let $\bLambda = g(\bSigma)=\bSigma^{-1}$, where $\bSigma\sim \inversewishart(\bS, \nu)$ and $\bLambda\sim \wishartdist(\bS, \nu)$. Then, $f(\bSigma)  = f(\bLambda) |J_g|$, where $J_g$ is the Jacobian matrix,  results in $f(\bSigma) = f(\bLambda) |J_g| = f(\bLambda)\abs{\bSigma}^{-(D+1)} $. }
% \xlongequal{ \mathrm{y}=\frac{1}{x}} \frac{\lambda^r}{\Gamma(r)} y^{-r-1} exp(- \frac{\lambda}{y})$ for $y>0$. }

\begin{example}[Wishart and inverse-Wishart]\label{example:wis_invwish}
Consider a random matrix $\rmG \in \real^{\nu \times m}$, each row of which is drawn independently from the distribution $\normal(\bzero, \bM)$, where $\bM \in \real^{m \times m}$ is symmetric positive definite. 
Then, as mentioned above, the probability distribution of the $m \times m$ random matrix $\rmG^\top \rmG$ is called the Wishart distribution with $\nu$ degrees of freedom, denoted as $\wishartdist(\bM, \nu)$. Moreover, the distribution of the matrix $(\rmG^\top \rmG)^{-1}$ is called the inverse Wishart distribution and is denoted by $\inversewishart(\bM^{-1}, \nu)$. 

Now suppose that $\rmG \in \real^{\nu \times m}$ , each entry $\rg_{ij}$ of which is drawn independently from $\normal(0, \frac{1}{m} )$. Then, the matrix $\rmG^\top \rmG$ follows the Wishart distribution $\wishartdist(\frac{1}{m} \bI_m, \nu)$ that satisfies:
\begin{equation}
\Exp \left[\trace(\rmG^\top \rmG)^{-1}\right] = \frac{m^2}{\nu - m - 1},
\end{equation}
since $(\rmG^\top \rmG)^{-1} \sim \inversewishart(m \bI_m, \nu)$.
\end{example}

\section{Matrix Decomposition}
This section introduces several matrix decomposition methods, which can be instrumental in proving theories related to linear models or least squares models.
\index{Positive definite}
\index{Positive definiteness}
\subsection{Cholesky Decomposition}\label{section:choleskydecomp}
Positive definiteness or positive semidefiniteness (Definition~\ref{definition:psd-pd-defini}) is one of the most desirable properties a matrix can have. In this section, we introduce decomposition techniques for positive definite matrices, with a focus on the well-known \textit{Cholesky decomposition}.
The Cholesky decomposition is named after a French military officer and mathematician, Andr\'{e}-Louis Cholesky (1875{\textendash}1918), who developed this method in his surveying work. It is primarily used to solve linear systems involving positive definite matrices.

To establish the existence of the Cholesky decomposition, we rely on the well-known Sylvester's criterion.
\begin{theoremHigh}[Sylvester's criterion]\label{sylvester1}\index{Sylvester's criterion}
Let $\bA_k=\bA[1:k, 1:k] \in \real^{k \times k}$, $k = 1,2,\ldots, n$, be the leading principal submatrices of the symmetric matrix $\bA \in \real^{n\times n}$. Then $\bA$ is positive definite if and only if $\det(\bA_k) > 0$, $k = 1, 2, \ldots, n$.
\end{theoremHigh}
The proof can be found, for example, in \citet{lu2021numerical}.

Here, we establish the existence of the Cholesky decomposition using an inductive approach. Alternative proofs also exist, such as those derived from the LU decomposition  \citep{lu2022matrix}.

\index{Cholesky decomposition}
\begin{theoremHigh}[Cholesky decomposition]\label{theorem:cholesky-factor-exist}
Every positive definite (PD) matrix $\bA\in \real^{n\times n}$ can be factored as 
$$
\bA = \bR^\top\bR,
$$
where $\bR \in \real^{n\times n}$ is an upper triangular matrix \textbf{with positive diagonal elements}. This decomposition is called the \textit{Cholesky decomposition}  of $\bA$, and $\bR$ is known as the \textit{Cholesky factor} or \textit{Cholesky triangle} of $\bA$.
Specifically, the Cholesky decomposition is unique (Corollary~\ref{corollary:unique-cholesky-main}).
\end{theoremHigh}
Alternatively, $\bA$ can be factored as $\bA=\bL\bL^\top$, where $\bL=\bR^\top$ is a lower triangular matrix \textit{with positive diagonals}.
\begin{proof}[of Theorem~\ref{theorem:cholesky-factor-exist}]
We will prove by induction that every $n\times n$ positive definite matrix $\bA$ has a decomposition $\bA=\bR^\top\bR$. The $1\times 1$ case is trivial by setting $R\triangleq\sqrt{A}$, so that $A=R^2$. 

Suppose any $k\times k$ PD matrix $\bA_k$ has a Cholesky decomposition. We must show that any $(k+1)\times(k+1)$ PD matrix $\bA_{k+1}$ can also be factored as this Cholesky decomposition, then we complete the proof.

For any $(k+1)\times(k+1)$ PD matrix $\bA_{k+1}$, write $\bA_{k+1}$ as
$
\bA_{k+1} \triangleq \footnotesize\begin{bmatrix}
\bA_k & \bb \\
\bb^\top & d
\end{bmatrix}.
$
Since $\bA_k$ is PD, by the inductive hypothesis, it admits a Cholesky decomposition  $\bA_k = \bR_k^\top\bR_k$. Define the upper triangular matrix
$
\bR_{k+1}\triangleq\footnotesize\begin{bmatrix}
\bR_k & \br\\
0 & s
\end{bmatrix}.
$
Then,
$$
\bR_{k+1}^\top\bR_{k+1} = 
\begin{bmatrix}
\bR_k^\top\bR_k & \bR_k^\top \br\\
\br^\top \bR_k & \br^\top\br+s^2
\end{bmatrix}.
$$
Therefore, if we can prove $\bR_{k+1}^\top \bR_{k+1} = \bA_{k+1}$ is the Cholesky decomposition of $\bA_{k+1}$ (which requires the value $s$ to be positive), then we complete the proof. That is, we need to prove
$$
\begin{aligned}
\bb &= \bR_k^\top \br 
\qquad \text{and}\qquad 
d = \br^\top\br+s^2.
\end{aligned}
$$
Since $\bR_k$ is nonsingular, we can solve uniquely for $\br$ and $s$: 
$$
\begin{aligned}
\br &= \bR_k^{-\top}\bb
\qquad \text{and}\qquad 
s = \sqrt{d - \br^\top\br} = \sqrt{d - \bb^\top\bA_k^{-1}\bb},
\end{aligned}
$$
where we assume $s$ is nonnegative. However, we need to further prove that $s$ is not only nonnegative, but also positive. Since $\bA_k$ is PD, from Sylvester's criterion (Theorem~\ref{sylvester1}), and the fact that if matrix $\bM$ has a block formulation: $\bM=\begin{bmatrixfoot}
\bA & \bB \\
\bC & \bD 
\end{bmatrixfoot}$, then $\det(\bM) = \det(\bA)\det(\bD-\bC\bA^{-1}\bB)$, we have
$$
\det(\bA_{k+1}) = \det(\bA_k)\det(d- \bb^\top\bA_k^{-1}\bb) =  \det(\bA_k)(d- \bb^\top\bA_k^{-1}\bb)>0.
$$
Because $ \det(\bA_k)>0$, we then obtain that $(d- \bb^\top\bA_k^{-1}\bb)>0$, and this implies $s>0$.
This completes the proof.
\end{proof}

\index{Uniqueness}
\begin{corollary}[Uniqueness of Cholesky decomposition\index{Uniqueness}]\label{corollary:unique-cholesky-main}
The Cholesky decomposition $\bA=\bR^\top\bR$ for any positive definite matrix $\bA\in \real^{n\times n}$ is unique.
\end{corollary}
\begin{proof}[of Corollary~\ref{corollary:unique-cholesky-main}]
Suppose the Cholesky decomposition is not unique. Then, there exist two distinct decompositions such that $\bA=\bR_1^\top\bR_1 = \bR_2^\top\bR_2$.  Rearranging, we obtain
$$
\bR_1\bR_2^{-1}= \bR_1^{-\top} \bR_2^\top.
$$
Since the inverse of an upper triangular matrix is also upper triangular, and the product of two upper triangular matrices is upper triangular, \footnote{Similarly, the inverse of a lower triangular matrix is lower triangular, and the product of two lower triangular matrices is also lower triangular.} we conclude that the left-hand side of the above equation is an upper triangular matrix, while the right-hand side is a lower triangular matrix. Consequently, $\bR_1\bR_2^{-1}= \bR_1^{-\top} \bR_2^\top$ must be a diagonal matrix, and $\bR_1^{-\top} \bR_2^\top= (\bR_1^{-\top} \bR_2^\top)^\top = \bR_2\bR_1^{-1}$.
Let $\bLambda \triangleq \bR_1\bR_2^{-1}= \bR_2\bR_1^{-1}$ be the diagonal matrix. We notice that the diagonal value of $\bLambda$ is the product of the corresponding diagonal values of $\bR_1$ and $\bR_2^{-1}$ (or $\bR_2$ and $\bR_1^{-1}$). Explicitly, writing the matrices as
$$
\bR_1=
\small
\begin{bmatrix}
r_{11} & r_{12} & \ldots & r_{1n} \\
0 & r_{22} & \ldots & r_{2n}\\
\vdots & \vdots & \ddots & \vdots\\
0 & 0 & \ldots & r_{nn}
\end{bmatrix},
\qquad 
\normalsize
\bR_2=
\small
\begin{bmatrix}
s_{11} & s_{12} & \ldots & s_{1n} \\
0 & s_{22} & \ldots & s_{2n}\\
\vdots & \vdots & \ddots & \vdots\\
0 & 0 & \ldots & s_{nn}
\end{bmatrix},
$$
we find that
$$
\begin{aligned}
\bR_1\bR_2^{-1}=
\small
\begin{bmatrix}
\frac{r_{11}}{s_{11}} & 0 & \ldots & 0 \\
0 & \frac{r_{22}}{s_{22}} & \ldots & 0\\
\vdots & \vdots & \ddots & \vdots\\
0 & 0 & \ldots & \frac{r_{nn}}{s_{nn}}
\end{bmatrix}
=
\small
\begin{bmatrix}
\frac{s_{11}}{r_{11}} & 0 & \ldots & 0 \\
0 & \frac{s_{22}}{r_{22}} & \ldots & 0\\
\vdots & \vdots & \ddots & \vdots\\
0 & 0 & \ldots & \frac{s_{nn}}{r_{nn}}
\end{bmatrix}
\normalsize
=\bR_2\bR_1^{-1}.
\end{aligned}
$$ 
Since both $\bR_1$ and $\bR_2$ have positive diagonals, this implies $r_{11}=s_{11}, r_{22}=s_{22}, \ldots, r_{nn}=s_{nn}$. 
Thus, we conclude that $\bLambda = \bR_1\bR_2^{-1}= \bR_2\bR_1^{-1}  =\bI$, which implies $\bR_1=\bR_2$, contradicting our initial assumption that the decomposition is not unique. Therefore, the Cholesky decomposition is unique.
\end{proof}

We now state some useful properties of  positive definite matrices. From the Cholesky decomposition follows a well-known characterization.

\begin{theoremHigh}\label{theorem:sylt2}
Let $\bA \in \real^{n\times n}$ be symmetric, and let $\bX \in \real^{n \times p}$ have full column rank. Then $\bX^\top \bA \bX$ is positive definite. In particular, any principal $p \times p$ submatrix
$$
\begin{bmatrix}
a_{i_1 i_1} & \ldots & a_{i_1 i_p} \\
\vdots & \ddots & \vdots \\
a_{i_p i_1} & \ldots & a_{i_p i_p}
\end{bmatrix} \in \real^{p \times p}, \quad 1 \leq p < n,
$$
is positive definite. From $p = 1$ it follows that all diagonal elements in $\bA$ are real positive.
\end{theoremHigh}
\begin{proof}[of Theorem~\ref{theorem:sylt2}]
Suppose $\bA$ is positive definite, $\bbeta \neq 0$, and $\balpha = \bX \bbeta$. Then since $\bX$ has full column rank, it follows that $\balpha  \neq \bzero$ and
$ \bbeta^\top (\bX^\top \bA \bX) \bbeta = \balpha^\top \bA \balpha > 0. $
The result now follows because any principal submatrix of $\bA$ can be written as $\bX^\top \bA \bX$, where the columns of $\bX$ are taken to be the columns $k = i_j$, $j = 1, 2, \ldots, p$, of the identity matrix.
\end{proof}

\begin{corollary}\label{corollary:sylt3}
The element of maximum magnitude of a real symmetric positive definite matrix $\bA  \in \real^{n\times n}$ lies on the diagonal.
\end{corollary}
\begin{proof}[of Corollary~\ref{corollary:sylt3}]
From Theorem~\ref{theorem:sylt2} and Sylvester's criterion it follows that
\begin{equation}
\det 
\left(\begin{bmatrix}
a_{ii} & a_{ij} \\
a_{ij} & a_{jj}
\end{bmatrix}
\right) = a_{ii} a_{jj} - {a_{ij}}^2 > 0, \quad 1 \leq i, j \leq n.
\end{equation}
Hence ${a_{ij}}^2 < a_{ii} a_{jj} \leq \max_{1 \leq i \leq n} a_{ii}^2$.
\end{proof}

\index{CR decomposition}
\index{Reduced row echelon form}
\subsection{CR and Rank Decomposition}\label{section:cr-decomposition}

The CR decomposition, as introduced in \citet{strang2021every, strang2022lu}, is presented as follows without a proof; see \citet{strang2022lu} for more details.
\begin{theoremHigh}[CR decomposition]\label{theorem:cr-decomposition}
Any rank-$r$ matrix $\bX \in \real^{n \times p}$ admits the following decomposition:
$$
\underset{n\times p}{\bX} = \underset{n\times r}{\bC} \gapthree \underset{r\times p}{\bR}
$$
where $\bC$ contains the first $r$ linearly independent columns of $\bX$, and $\bR$ is an $r\times p$ matrix used to reconstruct the columns of $\bX$ from the columns of $\bC$. In particular, $\bR$ is the  reduced row  echelon form (RREF) of $\bX$ without the zero rows.

This decomposition leads to a potential reduction or increase in storage requirements, transitioning from $np$ floating-point numbers to $r(n+p)$ floating-point numbers.
\end{theoremHigh}

The CR decomposition represents a particular instance of rank decomposition. To provide a rigorous demonstration of the existence of rank decomposition, we present the following theorem.

\index{Rank decomposition}
\begin{theoremHigh}[Rank decomposition]\label{theorem:rank-decomposition}
Any rank-$r$ matrix $\bX \in \real^{n \times p}$ admits the following decomposition:
$$
\underset{n\times p}{\bX }= \underset{n\times r}{\bD}\gapthree  \underset{r\times p}{\bF},
$$
where $\bD \in \real^{n\times r}$ has rank $r$, and $\bF \in \real^{r\times p}$ also has rank $r$, i.e., $\bD$ and $\bF$ have full rank $r$.

The storage for the decomposition is then reduced or potentially increased from $np$ floating-point numbers to $r(n+p)$ floating-point numbers.
\end{theoremHigh}
\begin{proof}[of Theorem~\ref{theorem:rank-decomposition}]
By ULV decomposition in Theorem~\ref{theorem:ulv-decomposition}, we can decompose $\bX$ by 
$$
\bX = \bU \begin{bmatrix}
\bL & \bzero \\
\bzero & \bzero 
\end{bmatrix}\bV.
$$
Let $\bU_0 = \bU_{:,1:r}$ and $\bV_0 = \bV_{1:r,:}$, i.e., $\bU_0$ contains only the first $r$ columns of $\bU$, and $\bV_0$ contains only the first $r$ rows of $\bV$. Then, we still have $\bX = \bU_0 \bL\bV_0$, where $\bU_0 \in \real^{n\times r}$ and $\bV_0\in \real^{r\times p}$. This is also known as the reduced ULV decomposition. Let \{$\bD = \bU_0\bL$ and $\bF =\bV_0$\}, or \{$\bD = \bU_0$ and $\bF =\bL\bV_0$\}, we find such rank decompositions.
\end{proof}

The rank decomposition is not unique. Even by elementary transformations, we have 
$$
\bX = 
\bE_1
\begin{bmatrix}
\bZ & \bzero \\
\bzero & \bzero 
\end{bmatrix}
\bE_2,
$$
where $\bE_1 \in \real^{n\times n}, \bE_2\in \real^{p\times p}$ represent elementary row and column operations, and $\bZ\in \real^{r\times r}$. The transformation is rather general, and there are dozens of these $\bE_1$, $\bE_2$, and $\bZ$. By similar construction on this decomposition as shown in the  proof above, we can recover another rank decomposition. 

\index{Rank decomposition}
Similarly, one can obtain matrices $\bD$ and $\bF$ through methods such as SVD, URV, CR, CUR, and various other decomposition algorithms.  
However, we may connect  different rank decompositions by the following lemma.
\begin{lemma}[Connection between rank decompositions]\label{lemma:connection-rank-decom}
For any two rank decompositions of $\bX=\bD_1\bF_1=\bD_2\bF_2$, there exists a nonsingular matrix $\bP$ such that 
$$
\bD_1 = \bD_2\bP
\qquad
\text{and}
\qquad 
\bF_1 = \bP^{-1}\bF_2.
$$
\end{lemma}
\begin{proof}[of Lemma~\ref{lemma:connection-rank-decom}]
Since $\bD_1\bF_1=\bD_2\bF_2$, we have $\bD_1\bF_1\bF_1^\top=\bD_2\bF_2\bF_1^\top$. It is evident that $\rank(\bF_1\bF_1^\top)=\rank(\bF_1)=r$ such that $\bF_1\bF_1^\top$ is a square matrix with full rank and thus is nonsingular. This implies $\bD_1=\bD_2\bF_2\bF_1^\top(\bF_1\bF_1^\top)^{-1}$. Let $\bP=\bF_2\bF_1^\top(\bF_1\bF_1^\top)^{-1}$, we have $\bD_1=\bD_2\bP$ and $\bF_1 = \bP^{-1}\bF_2$. 
\end{proof}

\subsection{QR Decomposition}\label{section:qr-decomposition}

In many applications, we are interested in the column space of a matrix $\bX=[\bx_1, \bx_2, \ldots, \bx_p] \in \real^{n\times p}$. The successive spaces spanned by the columns $\bx_1, \bx_2, \ldots$ of $\bX$ are
$$
\cspace([\bx_1])\,\,\,\, \subseteq\,\,\,\, \cspace([\bx_1, \bx_2]) \,\,\,\,\subseteq\,\,\,\, \cspace([\bx_1, \bx_2, \bx_3])\,\,\,\, \subseteq\,\,\,\, \ldots,
$$
where $\cspace([\ldots])$ is the subspace spanned by the vectors included in the brackets. 
Moreover, the notion of orthogonal or orthonormal bases within the column space plays a crucial role in various algorithms, allowing for efficient computations and interpretations.
The idea of QR decomposition involves the construction of a sequence of orthonormal vectors $\bq_1, \bq_2, \ldots$ that span the same successive subspaces. That is,
$$
\cspace([\bq_1])=\cspace([\bx_1]),\qquad \cspace([\bq_1, \bq_2])=\cspace([\bx_1, \bx_2]), \qquad  \cspace([\bq_1, \bq_2, \bq_3])=\cspace([\bx_1, \bx_2, \bx_3]),   \ldots.
$$
We illustrate the result of QR decomposition in the following theorem, the proof of which will be discussed in the sequel.
\index{QR decomposition}
\begin{theoremHigh}[QR decomposition]\label{theorem:qr-decomposition-in-ls}
Every $n\times p$ matrix $\bX=[\bx_1, \bx_2, \ldots , \bx_p]$ (whether independent or dependent columns) with $n\geq p$ admits the following decomposition:
$$
\bX = \bQ\bR,
$$
where 
\begin{enumerate}
\item \textbf{Reduced}: $\bQ$ is $\textcolor{mydarkblue}{n\times p}$ with orthonormal columns, and $\bR$ is an $\textcolor{mydarkblue}{p\times p}$ upper triangular matrix, known as the \textit{reduced QR decomposition};

\item \textbf{Full}: $\bQ$ is $\textcolor{mydarkblue}{n\times n}$ with orthonormal columns, and $\bR$ is an $\textcolor{mydarkblue}{n\times p}$ upper triangular matrix,  known as the \textit{full QR decomposition}.
If we further restrict the upper triangular matrix to be a square matrix, the full QR decomposition can be denoted by 
$$
\bX = \bQ\begin{bmatrix}
\bR_0\\
\bzero
\end{bmatrix},
$$
where $\bR_0$ is an $p\times p$ upper triangular matrix. 
\end{enumerate}
Specifically, when $\bX$ has full rank, i.e., $\bX$  has linearly independent columns, $\bR$ also exhibits linearly independent columns, and $\bR$ is nonsingular in the \textit{reduced} case. 
This implies that the diagonals of $\bR$ are nonzero. Under this condition, when we further restrict that elements on the diagonal of $\bR$ to be  positive, the \textit{reduced} QR decomposition is \textbf{unique}. The \textit{full} QR decomposition is normally not unique since the right-most $(n-p)$ columns in $\bQ$ can be arranged in any order.
\end{theoremHigh}

\begin{figure}[H]
\centering  
\vspace{-0.35cm}  
\subfigtopskip=2pt  
\subfigbottomskip=2pt  
\subfigcapskip=-5pt  
\subfigure[Reduced QR decomposition.]{\label{fig:qr-comparison-reduced}
\includegraphics[width=0.47\linewidth]{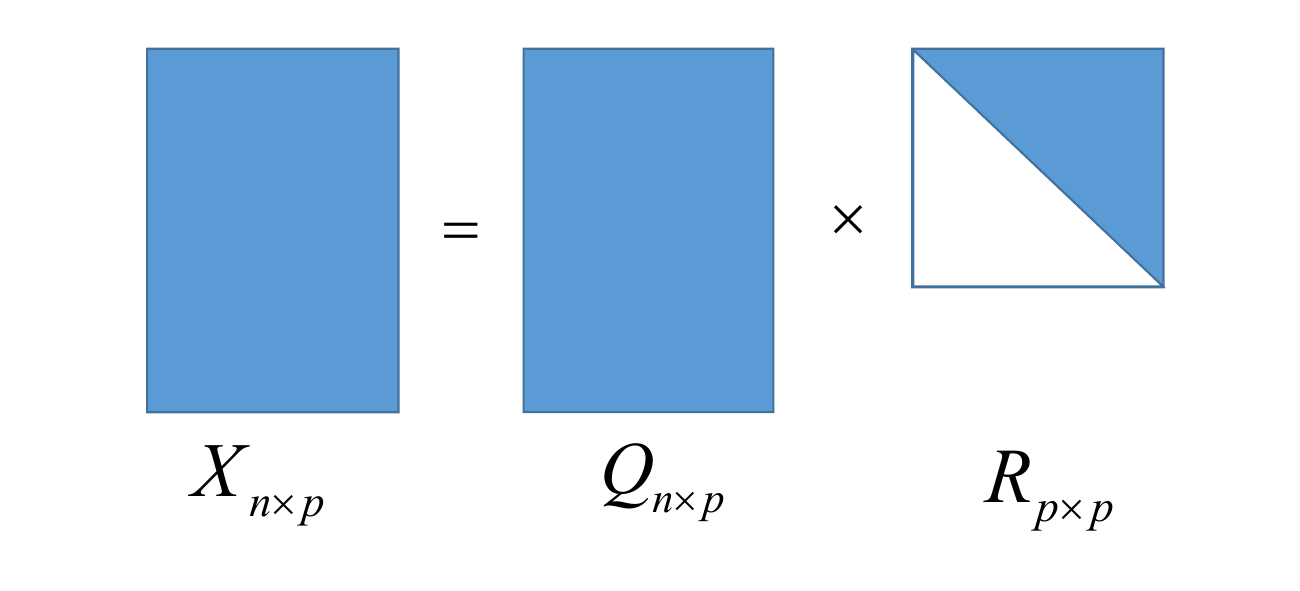}}
\quad 
\subfigure[Full QR decomposition.]{\label{fig:qr-comparison-full}
\includegraphics[width=0.47\linewidth]{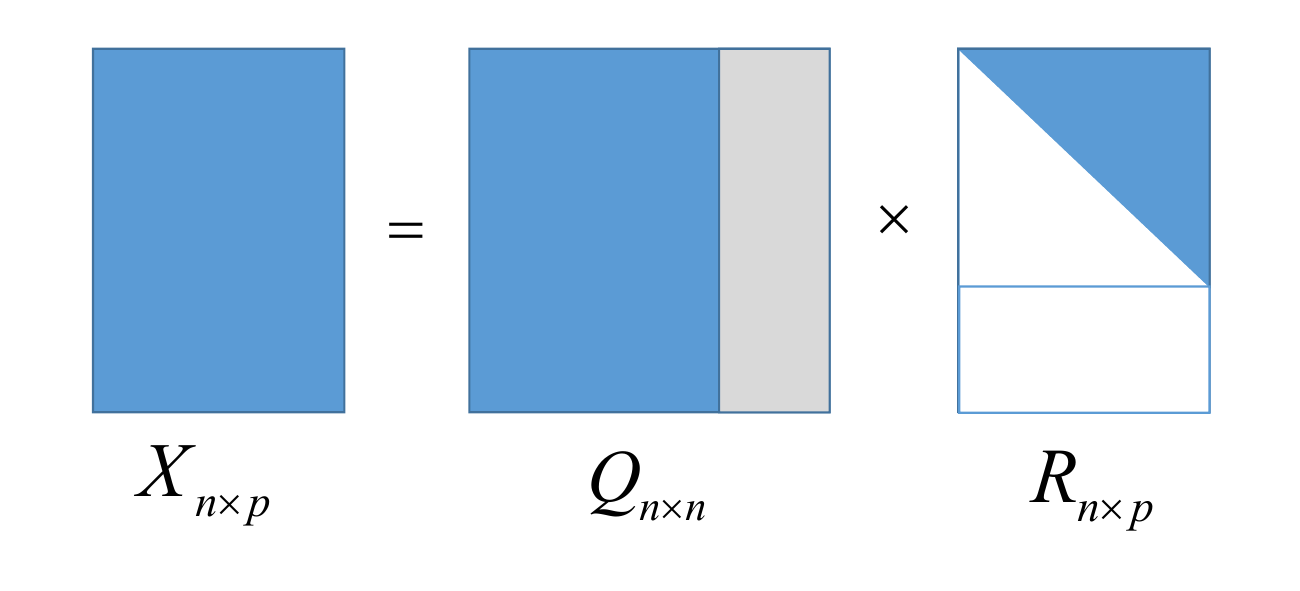}}
\caption{Comparison between the reduced and full QR decompositions. White entries are zero, and \textcolor{mylightbluetext}{blue} entries are not necessarily zero. \textcolor{mydarkgray}{Gray} columns denote silent columns.}
\label{fig:qr-comparison}
\end{figure}

If we obtain the reduced QR decomposition, a full QR decomposition of an $n\times p$ matrix with linearly independent columns goes further by appending additional $n-p$ orthonormal columns to $\bQ$, transforming it into an $n\times n$ orthogonal matrix. Simultaneously, $\bR$ is augmented with rows of zeros to attain an $n\times p$ upper triangular matrix.
We refer to the additional columns in $\bQ$ as \textit{silent columns} and the additional rows in $\bR$ as \textit{silent rows}.
The difference between the reduced and the full QR decomposition is shown in Figure~\ref{fig:qr-comparison}, where silent columns in $\bQ$ are denoted in gray, blank entries indicate zero elements, and blue entries are elements that are not necessarily zero.

%In this appendix, we provide the proof for Theorem~\ref{theorem:qr-decomposition-in-ls}, the existence of QR decomposition.
%In many applications, we are interested in the column space of a matrix $\bX=[\bx_1, \bx_2, \ldots, \bx_p] \in \real^{n\times p}$. The successive spaces spanned by the columns $\bx_1, \bx_2, \ldots$ of $\bX$ are
%$$
%\cspace([\bx_1])\,\,\,\, \subseteq\,\,\,\, \cspace([\bx_1, \bx_2]) \,\,\,\,\subseteq\,\,\,\, \cspace([\bx_1, \bx_2, \bx_3])\,\,\,\, \subseteq\,\,\,\, \ldots,
%$$
%where $\cspace([\ldots])$ is the subspace spanned by the vectors included in the brackets. The idea of QR decomposition is the construction of a sequence of orthonormal vectors $\bq_1, \bq_2, \ldots$ that span the same successive subspaces. That is,
%$$
%\cspace([\bq_1])=\cspace([\bx_1]),\qquad \cspace([\bq_1, \bq_2])=\cspace([\bx_1, \bx_2]), \qquad  \cspace([\bq_1, \bq_2, \bq_3])=\cspace([\bx_1, \bx_2, \bx_3]),   \ldots.
%$$

\index{QR decomposition}
\index{Project vector}
\index{Orthogonal projection}
\subsection*{Project a Vector Onto Another Vector}\label{section:project-onto-a-vector}
To achieve the QR decomposition, we first discuss how to project a vector onto another vector, based on which the Gram-Schmidt process is employed iteratively.

Projecting a vector $\ba$ onto another vector $\bb$ involves determining the vector on the line of $\bb$ that is closest to $\ba$.  
The resulting projection vector, denoted as $\widehatba$, is a scalar multiple of $\bb$. 
Let $\widehatba = \widehat{x} \bb$,  then $\ba-\widehatba$ is perpendicular to $\bb$, as shown in Figure~\ref{fig:project-line}. 
This leads to the following outcome about projecting vector  $\ba$ onto vector $\bb$:
$$
\text{$\ba-\widehatba$ is perpendicular to $\bb$, so $(\ba-\widehat{x}\bb)^\top\bb=0$: $\widehat{x}$ = $\frac{\ba^\top\bb}{\bb^\top\bb}$ and $\widehatba = \frac{\ba^\top\bb}{\bb^\top\bb}\bb = \frac{\bb\bb^\top}{\bb^\top\bb}\ba$.}
$$

%\begin{figure}[h!]
%	\centering
%	\includegraphics[width=0.5\textwidth]{./figures/projectline.pdf}
%	\caption{Projection onto the hyperplane of $C(\bX)$ and disturbed by noise.}
%	\label{fig:project-line}
%\end{figure}
%
%\begin{figure}[h!]
%	\centering
%	\includegraphics[width=0.5\textwidth]{./figures/projectspace.pdf}
%	\caption{Projection onto the hyperplane of $C(\bX)$ and disturbed by noise.}
%	\label{fig:project-space}
%\end{figure}

\begin{figure}[h]
\centering
\vspace{-0.35cm}
\subfigtopskip=2pt
\subfigbottomskip=2pt
\subfigcapskip=-5pt
\subfigure[Project onto a line.]{\label{fig:project-line}
\includegraphics[width=0.44\linewidth]{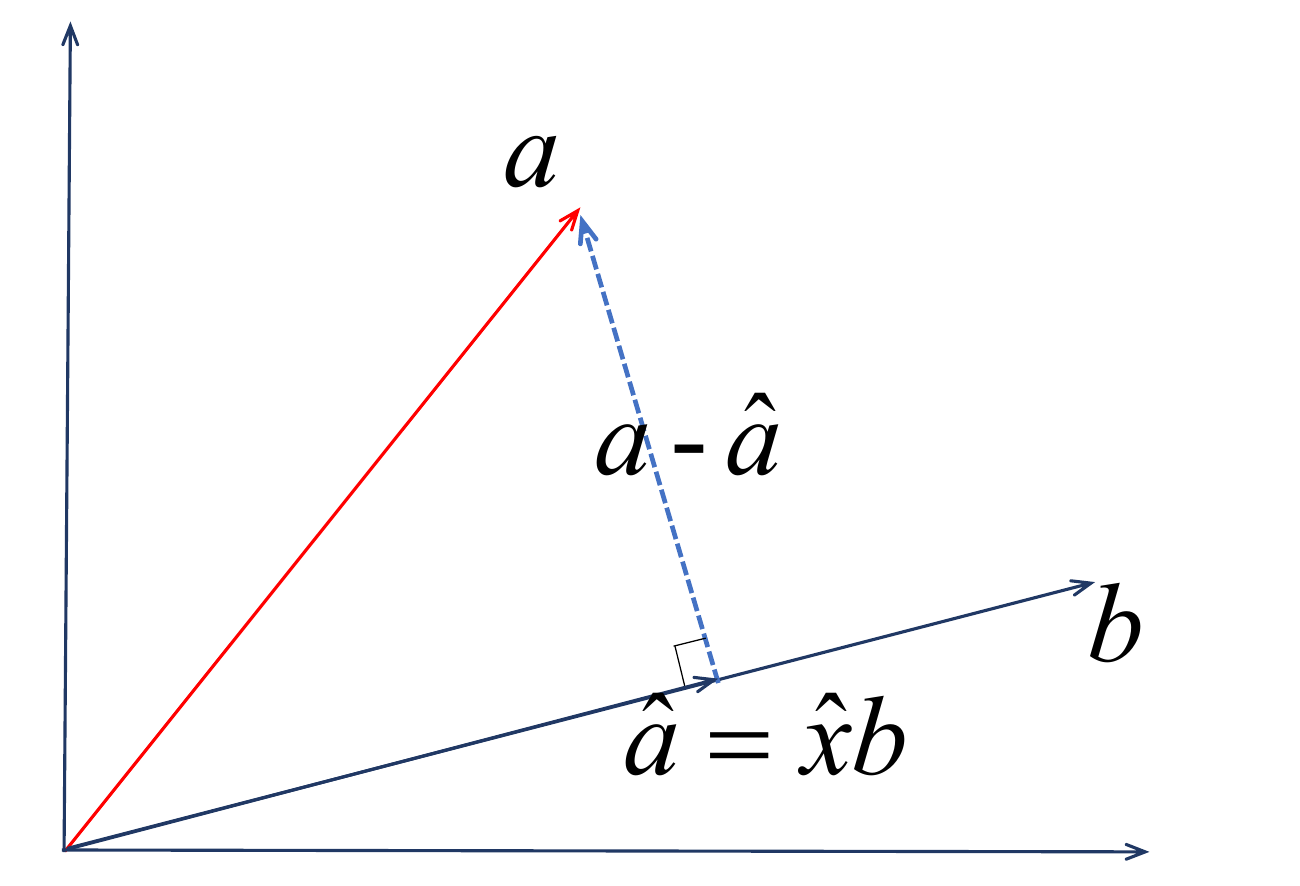}}
\quad 
\subfigure[Project onto a space.]{\label{fig:project-space}
\includegraphics[width=0.44\linewidth]{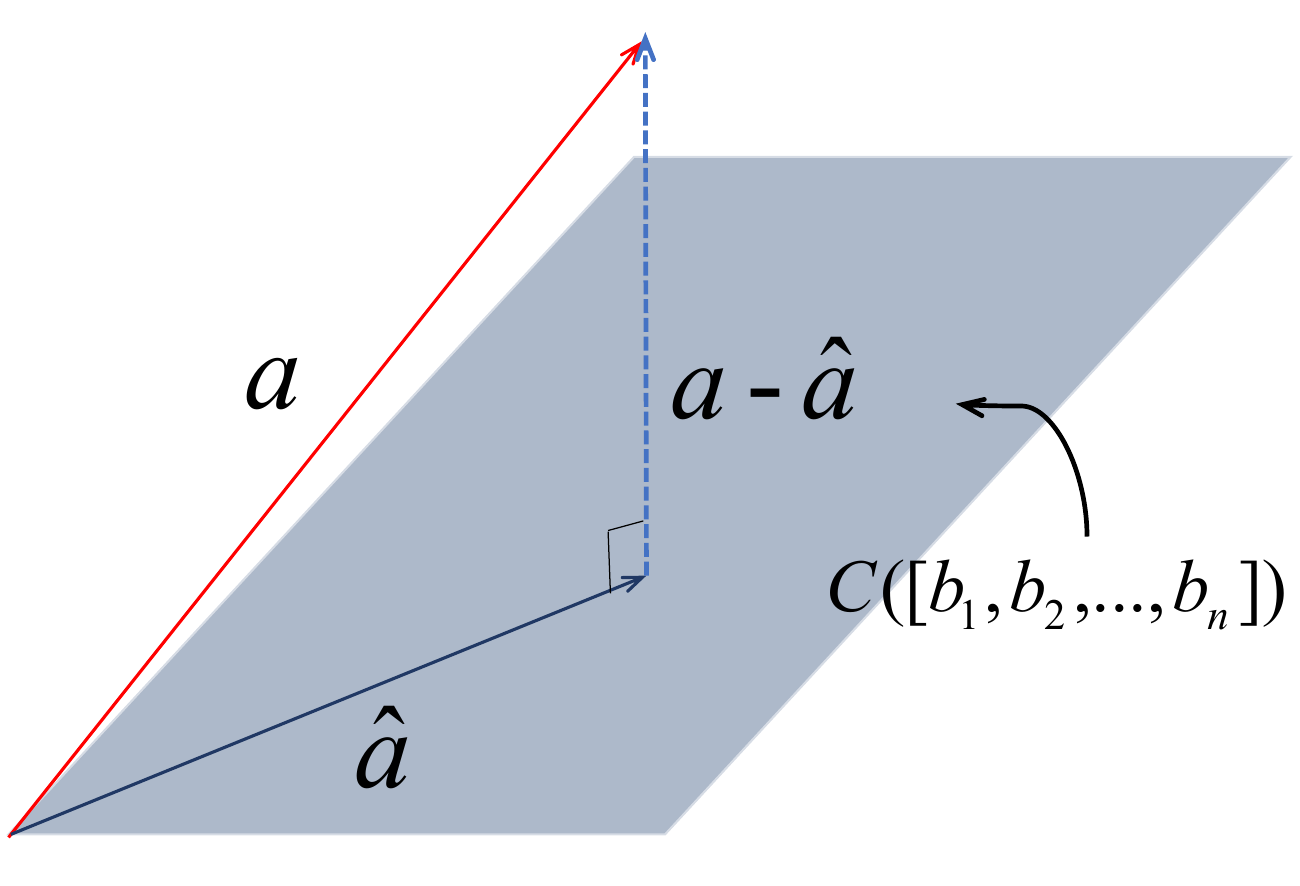}}
\caption{Project a vector onto a line and onto a space.}
\label{fig:projection-qr}
\end{figure}

\index{Orthogonal projection}
\subsection*{Project a Vector Onto a Plane}\label{section:project-onto-a-plane}
Projecting a vector $\ba$ onto a space spanned by $\bb_1, \bb_2, \ldots, \bb_p$ involves determining the vector closest to $\ba$ within the column space of $[\bb_1, \bb_2, \ldots, \bb_p]$. 
The resulting projection vector, denoted as $\widehat{\ba}$, is expressed as a linear combination of $\bb_1, \bb_2, \ldots, \bb_p$: $\widehat{\ba} = \widehat{x}_1\bb_1+ \widehat{x}_2\bb_2+\ldots+\widehat{x}_p\bb_p$.  
This scenario can be formulated as a least squares problem, wherein the normal equation $\bB^\top\bB\widehat{\bx} = \bB^\top\ba$ is solved, where $\bB=[\bb_1, \bb_2, \ldots, \bb_p]$ and $\widehat{\bx}=[\widehat{x}_1, \widehat{x}_2, \ldots, \widehat{x}_p]$.  
For each vector $\bb_i$, the projection of $\ba$ in the direction of $\bb_i$ can be similarly obtained by 
$$
\widehat{\ba}_i = \frac{\bb_i\bb_i^\top}{\bb_i^\top\bb_i}\ba, \gap \forall\, i \in \{1,2,\ldots, p\}.
$$
Let $\widehat{\ba}=\sum_{i=1}^{n}\widehatba_i$, this results in 
$$
\ba^\perp = (\ba-\widehat{\ba}) \perp \cspace(\bB),
$$
i.e., $(\ba-\widehat{\ba})$ is perpendicular to the column space of $\bB=[\bb_1, \bb_2, \ldots, \bb_p]$, as shown in Figure~\ref{fig:project-space}.
See Section~\ref{section:by-geometry-hat-matrix} for more details about  projection matrices.

\index{Gram-Schmidt process}
\subsection*{Existence of the QR Decomposition via the Gram-Schmidt Process}\label{section:gram-schmidt-process}
We proceed by establishing the Gram-Schmidt process through vector projection. 
Given three independent vectors $\ba_1, \ba_2, \ba_3$ and the space spanned by the three vectors, denoted as $\cspace{([\ba_1, \ba_2, \ba_3])}$, i.e., the column space of matrix $[\ba_1, \ba_2, \ba_3]$. We intend to construct three orthogonal vectors $\bb_1, \bb_2, \bb_3$ such that $\cspace{([\bb_1, \bb_2, \bb_3])}$ = $\cspace{([\ba_1, \ba_2, \ba_3])}$. 
We then normalize these orthogonal vectors by dividing each by its length, resulting in three orthonormal vectors: $\bq_1 = \frac{\bb_1}{\norm{\bb_1}}$, $\bq_2 = \frac{\bb_2}{\norm{\bb_2}}$, and $\bq_2 = \frac{\bb_2}{\norm{\bb_2}}$.

For the first vector, we directly set $\bb_1 = \ba_1$. The second vector, $\bb_2$, must be perpendicular to the first. 
This is achieved by considering the vector $\ba_2$ and subtracting its projection along $\bb_1$:
\begin{equation}
\begin{aligned}
\bb_2 &= \ba_2- \frac{\bb_1 \bb_1^\top}{\bb_1^\top\bb_1} \ba_2 = (\bI- \frac{\bb_1 \bb_1^\top}{\bb_1^\top\bb_1} )\ba_2   \qquad &(\text{Projection view})\\
&= \ba_2-  \underbrace{\frac{ \bb_1^\top \ba_2}{\bb_1^\top\bb_1} \bb_1}_{\widehatba_2}, \qquad &(\text{Combination view}) \nonumber
\end{aligned}
\end{equation}
where the first equation shows that $\bb_2$ is a multiplication of a matrix and $\ba_2$, i.e., project $\ba_2$ onto the orthogonal complement space of $\cspace{([\bb_1])}$. The second equation shows that $\ba_2$ is a linear combination of $\bb_1$ and $\bb_2$.
Clearly, the space spanned by $\bb_1$ and $\bb_2$ is the same space spanned by $\ba_1$ and $\ba_2$. The situation is shown in Figure~\ref{fig:gram-schmidt1}, in which we choose \textbf{the direction of $\bb_1$ as the $x$-axis in the Cartesian coordinate system}. $\widehatba_2$ is the projection of $\ba_2$ onto the line $\bb_1$. It can be clearly shown that the part of $\ba_2$ perpendicular to $\bb_1$ is $\bb_2 = \ba_2 - \widehatba_2$ from the figure.

For the third vector $\bb_3$, it must be perpendicular to both the $\bb_1$ and $\bb_2$, which is actually the vector $\ba_3$ subtracting its projection along the plane spanned by $\bb_1$ and $\bb_2$:
\begin{equation}\label{equation:gram-schdt-eq2}
\begin{aligned}
\bb_3 &= \ba_3- \frac{\bb_1 \bb_1^\top}{\bb_1^\top\bb_1} \ba_3 - \frac{\bb_2 \bb_2^\top}{\bb_2^\top\bb_2} \ba_3 = (\bI- \frac{\bb_1 \bb_1^\top}{\bb_1^\top\bb_1}  - \frac{\bb_2 \bb_2^\top}{\bb_2^\top\bb_2} )\ba_3   \qquad &(\text{Projection view})\\
&= \ba_3- \underbrace{\frac{ \bb_1^\top\ba_3}{\bb_1^\top\bb_1} \bb_1}_{\widehatba_3} - \underbrace{\frac{ \bb_2^\top\ba_3}{\bb_2^\top\bb_2}  \bb_2}_{\bar{\ba}_3},    \qquad &(\text{Combination view})
\end{aligned}
\end{equation}
where the first equation shows  that $\bb_3$ is a multiplication of a matrix and $\ba_3$, i.e., project $\ba_3$ onto the orthogonal complement space of $\cspace{([\bb_1, \bb_2])}$. The second equation shows that $\ba_3$ is a linear combination of $\bb_1$, $\bb_2$, and $\bb_3$. 
Again, it can be shown that the space spanned by $\bb_1, \bb_2, \bb_3$ is the same space spanned by $\ba_1, \ba_2, \ba_3$. The situation is shown in Figure~\ref{fig:gram-schmidt2}, in which we choose \textbf{the direction of $\bb_2$ as the $y$-axis of the Cartesian coordinate system}. $\widehatba_3$ is the projection of $\ba_3$ onto the line $\bb_1$, and $\bar{\ba}_3$ is the projection of $\ba_3$ onto the line $\bb_2$. It can be shown that the part of $\ba_3$ perpendicular to both $\bb_1$ and $\bb_2$ is $\bb_3=\ba_3-\widehatba_3-\bar{\ba}_3$ from the figure.

Finally, we normalize each vector by dividing its length, resulting in three orthonormal vectors:  $\bq_1 = \frac{\bb_1}{\norm{\bb_1}}$, $\bq_2 = \frac{\bb_2}{\norm{\bb_2}}$, and $\bq_2 = \frac{\bb_2}{\norm{\bb_2}}$.

\begin{figure}[h]
\centering  
\vspace{-0.35cm} 
\subfigtopskip=2pt 
\subfigbottomskip=2pt 
\subfigcapskip=-5pt 
\subfigure[Project $\ba_2$ onto the space perpendicular to $\bb_1$.]{\label{fig:gram-schmidt1}
\includegraphics[width=0.44\linewidth]{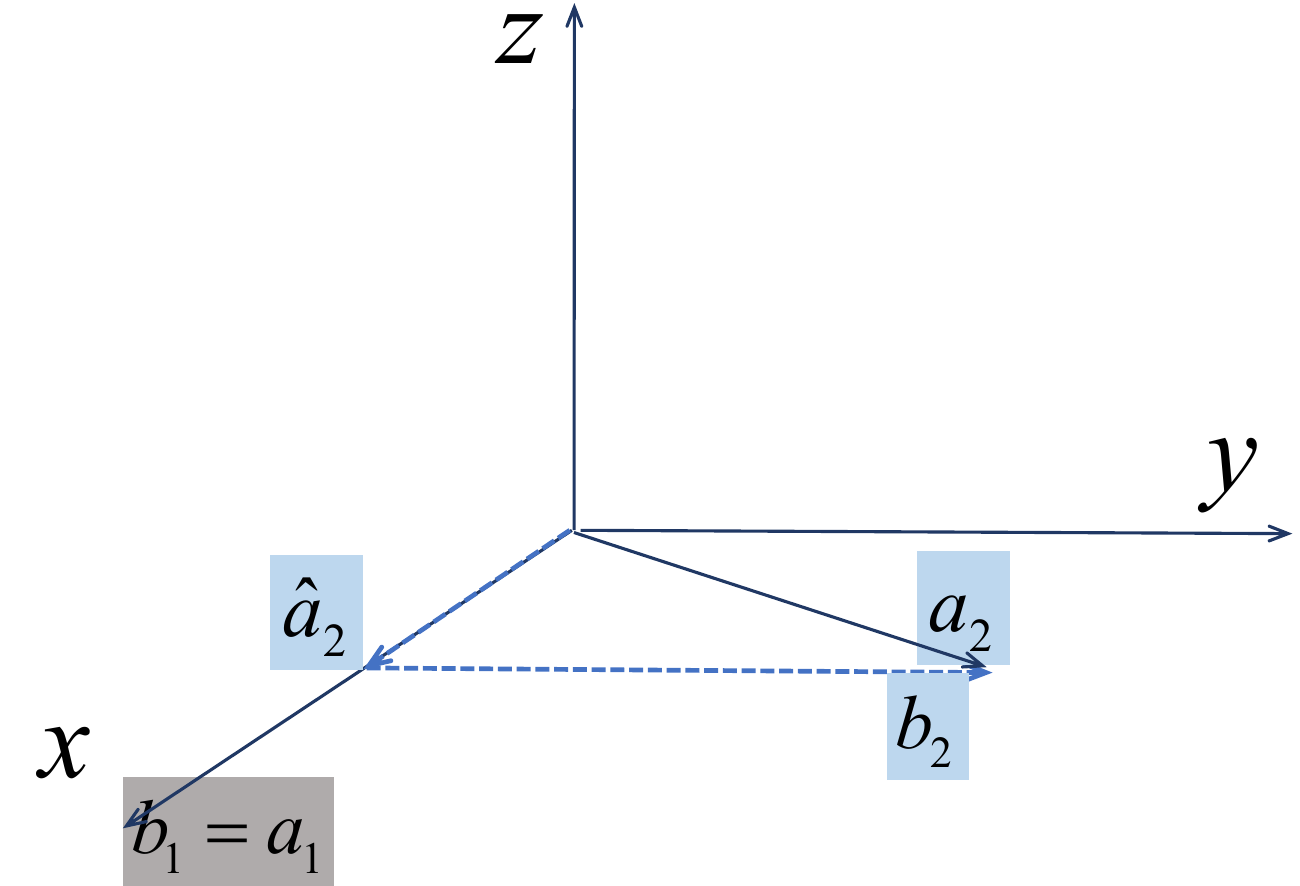}}
\quad 
\subfigure[Project $\ba_3$ onto the space perpendicular to $\bb_1$ and $\bb_2$.]{\label{fig:gram-schmidt2}
\includegraphics[width=0.44\linewidth]{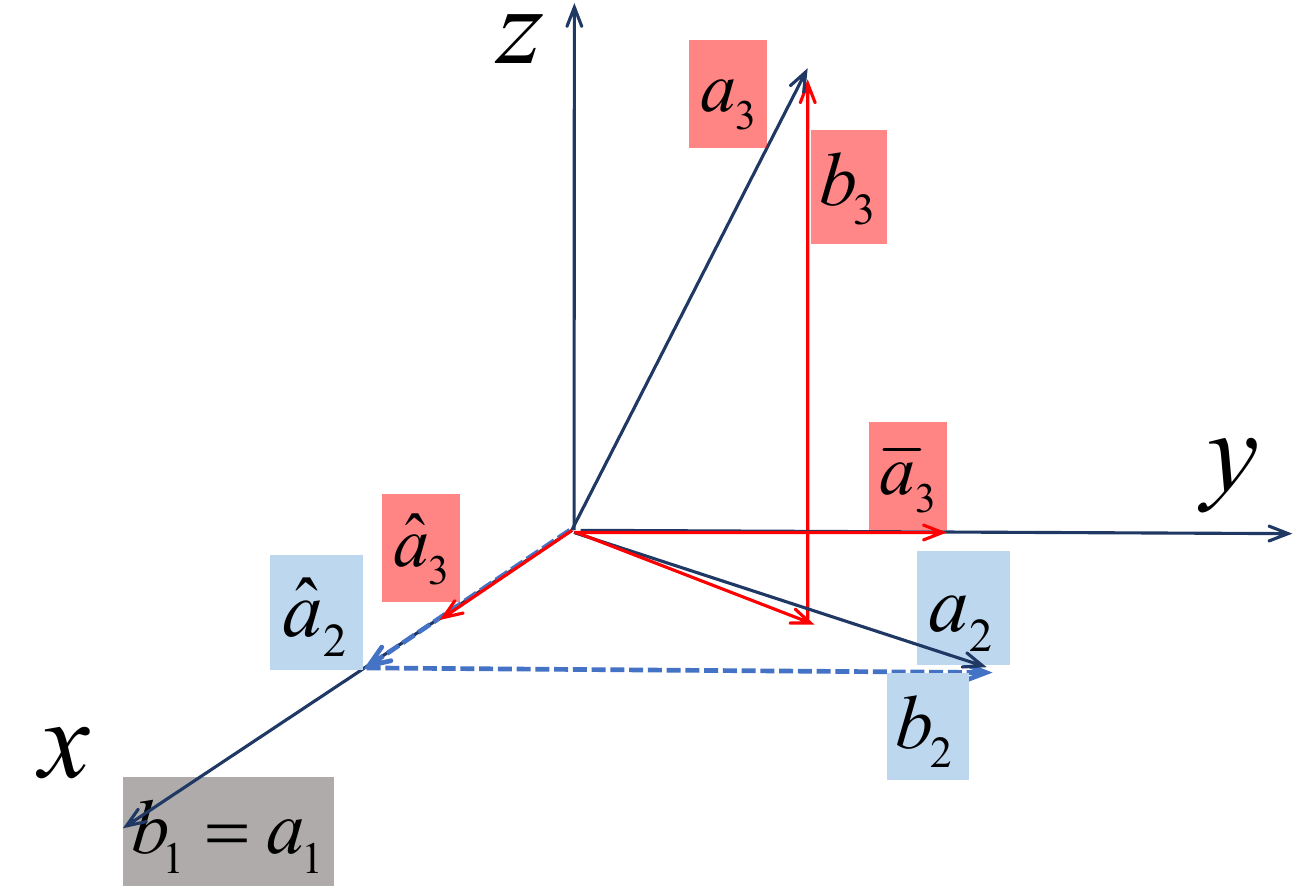}}
\caption{Gram-Schmidt process.}
\label{fig:gram-schmidt-12}
\end{figure}

This idea can be extended to a set of vectors rather than only three. And we refer to this process as the \textit{Gram-Schmidt process}. After this process, the matrix $\bX$ will be triangularized. The method is named after Jørgen Pedersen Gram and Erhard Schmidt, but it appeared earlier in the work of Pierre-Simon Laplace in the theory of Lie group decomposition.

The Gram–Schmidt process is not the sole algorithm for obtaining the QR decomposition.
There are several other QR decomposition algorithms available, such as \textit{Householder reflections} and \textit{Givens rotations}, which exhibit greater robustness in the presence of round-off errors; see Section~\ref{section:ls_qr_gen}.  
These QR decomposition methods may also alter the order in which the columns of $\bX$ are processed.

\subsection*{Properties of the QR Decomposition}
\paragrapharrow{Orthonormal basis.}
For any matrix $\bX$, we have the property: $\nspace(\bX^\top)$ is the orthogonal complement of the column space $\cspace(\bX)$ in $\real^n$: $\dim(\nspace(\bX^\top))+\dim(\cspace(\bX))=n$;
This relationship is known  as the rank-nullity theorem, and its proof can be found in Theorem~\ref{theorem:fundamental-linear-algebra}.
Specifically, QR decomposition yields a basis for this subspace. 
In singular value decomposition, we will also find the basis for $\nspace(\bX)$ and $\cspace(\bX^\top)$. 

\index{Orthonormal basis}
\begin{theoremHigh}[Orthonormal basis in $\real^n$]\label{theorem:qr-four-orthonormal-Basis}
Given the full QR decomposition of $\bX\in\real^{n\times p}$ with full rank $p$, we have the following property:
\begin{itemize}
\item $\{\bq_1,\bq_2 \ldots,\bq_p\}$ is an orthonormal basis of $\cspace(\bX)$;
\item$\{\bq_{p+1},\bq_{p+2}, \ldots,\bq_n\}$ is an orthonormal basis of $\nspace(\bX^\top)$. 
\end{itemize}
\end{theoremHigh}
\begin{proof}[of Theorem~\ref{theorem:qr-four-orthonormal-Basis}]
Following from the Gram-Schmidt process, it is trivial that $\spn\{\bx_1, \bx_2, \ldots, \bx_k\}$ is equal to $\spn\{\bq_1, \bq_2, \ldots, \bq_k\}$ for all $k\in\{1, 2, \ldots, p\}$. Thus, $\cspace(\bX) = \spn\{\bx_1, \bx_2, \ldots, \bx_p\}=\spn\{\bq_1, \bq_2, \ldots, \bq_p\}$, and $\{\bq_1, \bq_2, \ldots, \bq_p\}$ is an orthonormal basis for the column space of $\bX$. 
Additionally, we have  $ \nspace(\bX^\top) \bot\cspace(\bX)$, and $\dim(\nspace(\bX^\top))=n-\dim(\cspace(\bX))=n-p$. 
Since the space spanned by $\{\bq_{p+1},\bq_{p+2}, \ldots,\bq_n\}$ is also perpendicular to $\cspace(\bX)$ with dimension $n-p$, thus, $\{\bq_{p+1},\bq_{p+2}, \ldots,\bq_n\}$ is an orthonormal basis for $\nspace(\bX^\top)$.
\end{proof}

\paragrapharrow{Uniqueness of the QR decomposition.}%\label{section:nonunique-qr}
The QR decomposition is generally not unique. 

\begin{example}[Non-uniqueness of the QR decomposition\index{Uniqueness}]
Suppose the matrix $\bX$ is given by 
$$
\bX = 
\begin{bmatrix}
4 & 1 \\
3 & 2
\end{bmatrix}.
$$
The QR decomposition of $\bX$ can be obtained by 
$$
\begin{aligned}
\bX &= \bQ_1\bR_1=
\begin{bmatrix}
0.8 & 0.6 \\
0.6 & -0.8
\end{bmatrix}
\begin{bmatrix}
5 & 2 \\
0 & -1
\end{bmatrix}
&=& \bQ_2\bR_2=
\begin{bmatrix}
0.8 & -0.6 \\
0.6 & 0.8
\end{bmatrix}
\begin{bmatrix}
5 & 2 \\
0 & 1
\end{bmatrix}\\
&=\bQ_3\bR_3=
\begin{bmatrix}
-0.8 & -0.6 \\
-0.6 & 0.8
\end{bmatrix}
\begin{bmatrix}
-5 & -2 \\
0 & 1
\end{bmatrix}
&=& \bQ_4\bR_4=
\begin{bmatrix}
-0.8 & 0.6 \\
-0.6 & -0.8
\end{bmatrix}
\begin{bmatrix}
-5 & -2 \\
0 & -1
\end{bmatrix}.
\end{aligned}
$$
Thus, the QR decomposition of $\bX$ is not unique.
\end{example}

However, the uniqueness of the \textit{reduced} QR decomposition for a full column rank matrix $\bX$ is guaranteed when $\bR$ has positive diagonals.
\begin{theoremHigh}[Uniqueness of the reduced QR decomposition]\label{theorem:unique-qr}
Suppose matrix $\bX$ is an $n\times p$ matrix with full column rank $p$ and $n\geq p$. Then, the \textit{reduced} QR decomposition is unique if the main diagonal values of $\bR$ are positive.
\end{theoremHigh}
\index{Uniqueness}
\begin{proof}[of Theorem~\ref{theorem:unique-qr}]
Suppose the \textit{reduced} QR decomposition is not unique, we can complete it into a \textit{full} QR decomposition, then we can find two such full decompositions satisfying $\bX=\bQ_1\bR_1 = \bQ_2\bR_2$. 
This  implies $\bR_1 = \bQ_1^{-1}\bQ_2\bR_2 \triangleq \bV \bR_2$, where $\bV\triangleq \bQ_1^{-1}\bQ_2$ is an orthogonal matrix. Write out the equation, we have 
$$
\begin{aligned}
\bR_1 &= 
\begin{bmatrix}
r_{11} & r_{12}& \dots & r_{1p}\\
& r_{22}& \dots & r_{2p}\\
&       &    \ddots  & \vdots \\
\multicolumn{2}{c}{\raisebox{1.3ex}[0pt]{\Huge0}} & & r_{pp} \\
\bzero & \bzero &\ldots & \bzero
\end{bmatrix}=
\begin{bmatrix}
v_{11}& v_{12} & \ldots & v_{1n}\\
v_{21} & v_{22} & \ldots & v_{2n}\\
\vdots & \vdots & \ddots & \vdots\\
v_{n1} & v_{n2} & \ldots & v_{\textcolor{black}{nn}}
\end{bmatrix}
\begin{bmatrix}
s_{11} & s_{12}& \dots & s_{1p}\\
& s_{22}& \dots & s_{2p}\\
&       &    \ddots  & \vdots \\
\multicolumn{2}{c}{\raisebox{1.3ex}[0pt]{\Huge0}} & & s_{pp} \\
\bzero & \bzero &\ldots & \bzero
\end{bmatrix}= \bV\bR_2,
\end{aligned}
$$
This implies 
$$
r_{11} = v_{11} s_{11}, \qquad v_{21}=v_{31}=v_{41}=\ldots=v_{n1}=0.
$$
Since $\bV$ contains mutually orthonormal columns, and the first column of $\bV$ is of norm 1. Thus, $v_{11} = \pm 1$.
We notice that $r_{ii}> 0$ and $s_{ii}> 0$ for $i\in \{1,2,\ldots,p\}$ by assumption such that $r_{11}> 0$ and $s_{11}> 0$, implying that  $v_{11}$ can only be positive 1. Since $\bV$ is an orthogonal matrix, we also have 
$$
v_{12}=v_{13}=v_{14}=\ldots=v_{1n}=0. 
$$
Applying this process to the submatrices of $\bR_1$, $\bV$, and $\bR_2$, we will find the upper-left submatrix of $\bV$ is the identity matrix: $\bV[1:p,1:p]=\bI_p$ such that $\bR_1=\bR_2$. This implies $\bQ_1[:,1:p]=\bQ_2[:,1:p]$ and leads to a contradiction. And thus the reduced QR decomposition is unique.
We complete the proof.
\end{proof}

\subsection*{LQ Decomposition}\label{section:lq-decomp}
We have proved the existence of the QR decomposition via the Gram-Schmidt process, in which case we are interested in the column space of a matrix $\bX=[\bx_1, \bx_2, \ldots, \bx_p] \in \real^{n\times p}$. The successive spaces spanned by the columns $\bx_1, \bx_2, \ldots$ of $\bX$ are
$$
\cspace([\bx_1])\,\,\,\, \subseteq\,\,\,\, \cspace([\bx_1, \bx_2]) \,\,\,\,\subseteq\,\,\,\, \cspace([\bx_1, \bx_2, \bx_3])\,\,\,\, \subseteq\,\,\,\, \ldots,
$$
The concept behind QR decomposition involves generating a sequence of orthonormal vectors $\bq_1, \bq_2, \ldots$, spanning the same successive subspaces:
$$
\left\{\cspace([\bq_1])=\cspace([\bx_1]) \right\}\,\,\,\, \subseteq\,\,\,\, \{\cspace([\bq_1, \bq_2])=\cspace([\bx_1, \bx_2])\} \,\,\,\, \subseteq\,\,\,\, \ldots,
$$
However, in many applications (see \citet{schilders2009solution}),  interest extends to the row space of a matrix $\bY=[\by_1^\top; \by_2^\top; \ldots ;\by_n^\top] \in \real^{n\times p}$, where, abusing the notation, $\by_i$ denotes the $i$-th row of $\bY$. The successive spaces spanned by the rows $\by_1, \by_2, \ldots$ of $\bY$ are
$$
\cspace([\by_1])\,\,\,\, \subseteq\,\,\,\, \cspace([\by_1, \by_2]) \,\,\,\,\subseteq\,\,\,\, \cspace([\by_1, \by_2, \by_3])\,\,\,\, \subseteq\,\,\,\, \ldots.
$$
The QR decomposition thus has a counterpart that identifies the orthogonal row space.
By applying QR decomposition on $\bY^\top = \bQ_0\bR$, we recover the LQ decomposition of the matrix $\bY = \bL \bQ$, where $\bQ = \bQ_0^\top$ and $\bL = \bR^\top$. The LQ decomposition is helpful in demonstrating the existence of the UTV decomposition in Section~\ref{section:utv_ls}.

\index{LQ decomposition}
\begin{theoremHigh}[LQ decomposition]\label{theorem:lq-decomposition}
Every $n\times p$ matrix $\bY$ (whether linearly independent or dependent rows) with $p\geq n$ admits the following decomposition:
$$
\bY = \bL\bQ,
$$
where 
\begin{enumerate}
\item \textbf{Reduced}: $\bL$ is an $n\times n$ lower triangular matrix and $\bQ$ is $n\times p$ with orthonormal rows. This is known as the \textit{reduced LQ decomposition};

\item \textbf{Full}: $\bL$ is an $n\times p$ lower triangular matrix and $\bQ$ is $p\times p$ with orthonormal rows. This is known as the \textit{full LQ decomposition}. If we further restrict the lower triangular matrix to be a square matrix, the full LQ decomposition can be denoted as 
$$
\bY = \begin{bmatrix}
\bL_0 & \bzero
\end{bmatrix}\bQ,
$$
where $\bL_0$ is an $n\times n$ square lower triangular matrix.
\end{enumerate}
\end{theoremHigh}

Similarly, a comparison between the reduced and full LQ decomposition is shown in Figure~\ref{fig:lq-comparison}.

\begin{figure}[H]
\centering  
\vspace{-0.35cm} 
\subfigtopskip=2pt 
\subfigbottomskip=2pt
\subfigcapskip=-5pt
\subfigure[Reduced LQ decomposition.]{\label{fig:lqhalf}
\includegraphics[width=0.47\linewidth]{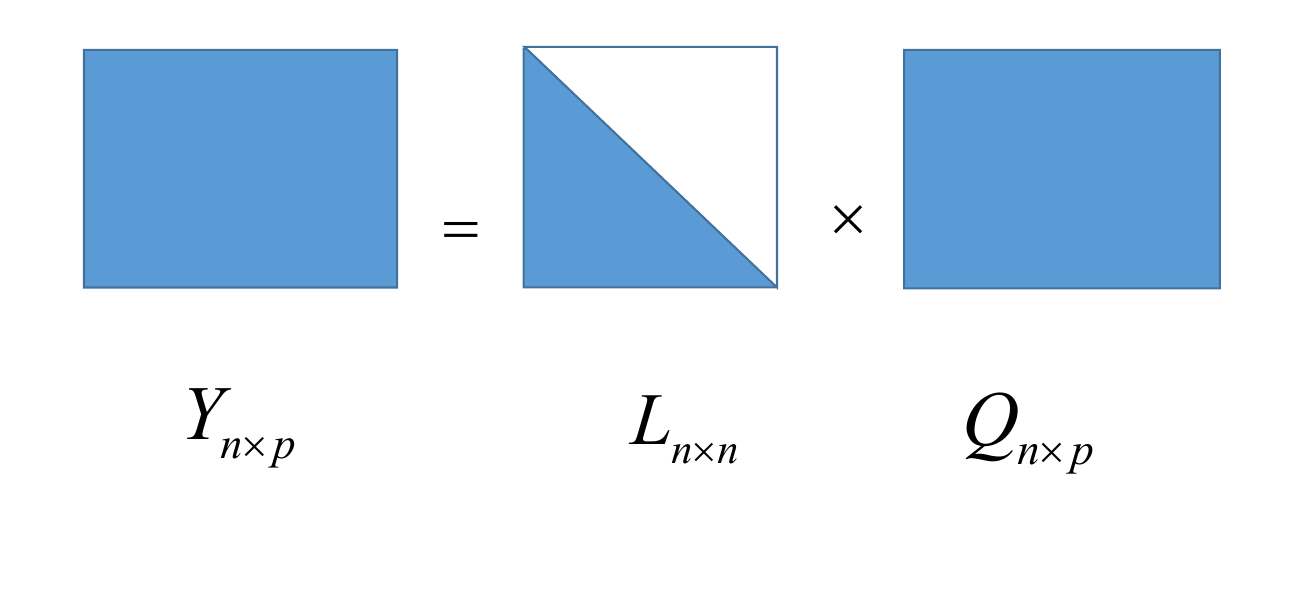}}
\quad 
\subfigure[Full LQ decomposition.]{\label{fig:lqall}
\includegraphics[width=0.47\linewidth]{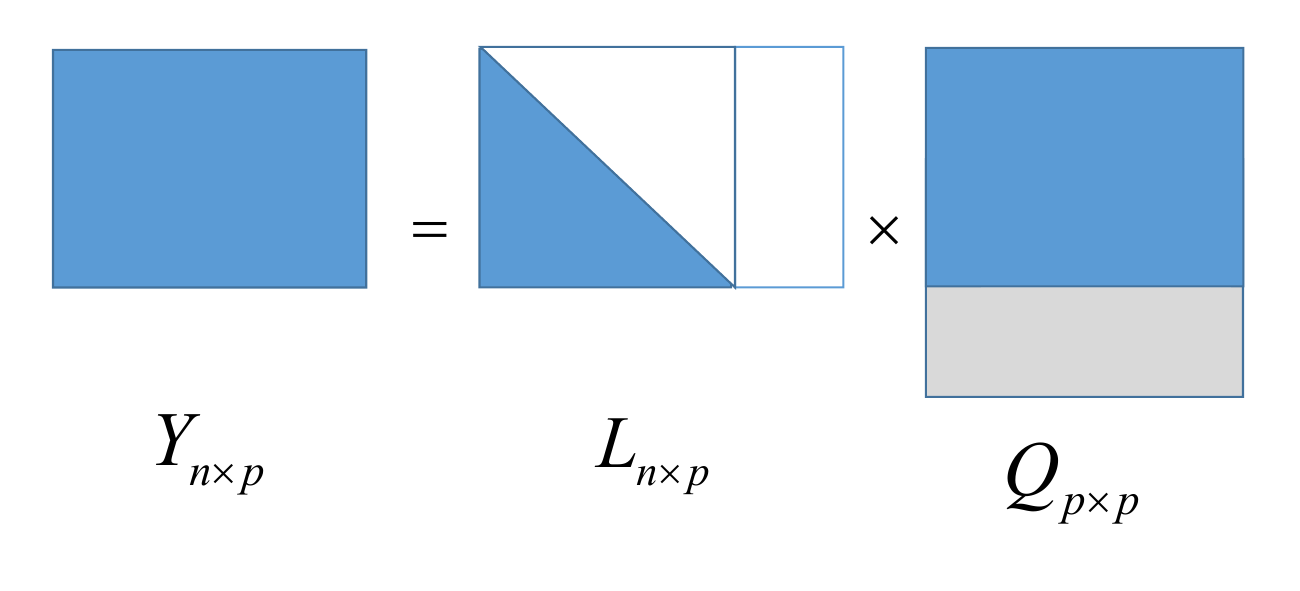}}
\caption{Comparison between the reduced and full LQ decomposition. White entries are zero, and blue entries are not necessarily zero. Gray columns denote silent rows.}
\label{fig:lq-comparison}
\end{figure}

\subsection{Schur and Spectral Decomposition}\label{section:spectral-decomposition}
\begin{theoremHigh}[Schur decomposition]\label{theorem:schur-decomposition}
Any real square matrix $\bX\in \real^{n\times n}$ with real eigenvalues admits the following decomposition:
$$
\bX = \bQ\bU\bQ^\top,
$$
where $\bQ$ is an orthogonal matrix, and $\bU$ is an upper triangular matrix. That is, any square matrix $\bX$ with real eigenvalues can be triangularized.
\end{theoremHigh}
\index{Schur decomposition}

\index{Symmetric matrix}
When dealing with a symmetric matrix $\bX=\bX^\top$, we find that $\bQ\bU\bQ^\top = \bQ\bU^\top\bQ^\top$. 
Consequently, $\bU$ is a diagonal matrix. And this diagonal matrix actually contains the eigenvalues of $\bX$. All the columns of $\bQ$ are eigenvectors of $\bX$. 
This leads us to the conclusion that symmetric matrices are inherently diagonalizable, even in the presence of repeated eigenvalues.

Moreover, the matrix $\bX$ and $\bU$ are in the notion of similar matrices.
\begin{definition}[Similar matrices]
For any nonsingular matrix $\bP$, the matrices $\bX$ and $\bP\bX\bP^{-1}$ are called \textit{similar matrices}.
\end{definition}
\begin{lemma}[Eigenvalue, trace and rank of similar matrices\index{Trace}\index{Similar matrices}]\label{lemma:eigenvalue-similar-matrices}
Given a nonsingular matrix $\bP$,
any eigenvalue of $\bX$ is also an eigenvalue of $\bP\bX\bP^{-1}$. 
The converse is also true that any eigenvalue of $\bP\bX\bP^{-1}$ is likewise an eigenvalue of $\bX$. 
In other words, $\Lambda(\bX) = \Lambda(\bP\bX\bP^{-1})$, where $\Lambda(\bZ)$ represents the spectrum of matrix $\bZ$ (Definition~\ref{definition:spectrum}).

Furthermore, the trace and rank of $\bX$ are equal to those of matrix $\bP\bX\bP^{-1}$ for any nonsingular matrix $\bP$.
\end{lemma}
\begin{proof}[of Lemma~\ref{lemma:eigenvalue-similar-matrices}]
For any eigenpair $(\lambda, \bbeta)$ of $\bX$, we have $\bX\bbeta =\lambda \bbeta$. Consequently, $\lambda \bP\bbeta = \bP\bX\bP^{-1} \bP\bbeta$, demonstrating that $\bP\bbeta$ is an eigenvector of $\bP\bX\bP^{-1}$ associated with $\lambda$.

Similarly, for any eigenpair $(\lambda, \bbeta)$ of $\bP\bX\bP^{-1}$, we have $\bP\bX\bP^{-1} \bbeta = \lambda \bbeta$. Then, $\bX\bP^{-1} \bbeta = \lambda \bP^{-1}\bbeta$, indicating that $\bP^{-1}\bbeta$ is an eigenvector of $\bX$ corresponding to $\lambda$. 

Regarding the the trace of $\bP\bX\bP^{-1}$, we can establish that $\trace(\bP\bX\bP^{-1}) = \trace(\bX\bP^{-1}\bP) = \trace(\bX)$, where the first equality comes from the fact that the
trace of a product is invariant under cyclical permutations of the factors:
\begin{equation}
\trace(\bA\bB\bC) = \trace(\bB\bC\bA) = \trace(\bC\bA\bB), \nonumber
\end{equation}
if all $\bA\bB\bC$, $\bB\bC\bA$, and $\bC\bA\bB$ exist.

Regarding the  rank of $\bP\bX\bP^{-1}$, we separate it into two claims as follows.
\paragraph{Rank claim 1: $\rank(\bZ\bX)=\rank(\bX)$ if $\bZ$ is nonsingular.}
We will begin by demonstrating that $\rank(\bZ\bX)=\rank(\bX)$ if $\bZ$ is nonsingular. 
Consider any vector $\bn$ in the null space of $\bX$, that is, $\bX\bn = \bzero$. 
Consequently, $\bZ\bX\bn = \bzero$, that is, $\bn$ also resides in the null space of $\bZ\bX$. 
This, in turn, implies  $\nspace(\bX)\subseteq \nspace(\bZ\bX)$.

Conversely, for any vector $\bmm$ in the null space of $\bZ\bX$, i.e., $\bZ\bX\bmm = \bzero$, we can deduce that $\bX\bmm = \bZ^{-1} \bzero=\bzero$. That is, $\bmm$  also lies in the null space of $\bX$. And this indicates $\nspace(\bZ\bX)\subseteq \nspace(\bX)$.

Combining the two arguments presented above leads to the following conclusion:
$$
\nspace(\bX) = \nspace(\bZ\bX)\quad  \longrightarrow \quad \rank(\bZ\bX)=\rank(\bX).
$$

\paragraph{Rank claim 2: $\rank(\bX\bZ)=\rank(\bX)$ if $\bZ$ is nonsingular.}
We observe that the row rank of a matrix is equivalent to its column rank (Lemma~\ref{lemma:equal-dimension-rank}). 
Therefore, $\rank(\bX\bZ) = \rank(\bZ^\top\bX^\top)$. Since $\bZ^\top$ is nonsingular, as per claim 1, we can conclude that $\rank(\bZ^\top\bX^\top) = \rank(\bX^\top) = \rank(\bX)$, where the last equality follows again from the fact that the row rank is equal to the column rank for any matrix. This establishes that $\rank(\bX\bZ)=\rank(\bX)$, as claimed.

Since $\bP$ and $\bP^{-1}$ are nonsingular, we can conclude  that $\rank(\bP\bX\bP^{-1}) = \rank(\bX\bP^{-1}) = \rank(\bX)$, where the first equality follows from claim 1, and the second equality follows from claim 2. We complete the proof.
\end{proof}

\index{Determinant}
\subsection*{Existence of the Schur Decomposition}

To prove Theorem~\ref{theorem:schur-decomposition}, we need to use the following lemma.
\begin{lemma}[Submatrix with same eigenvalue]\label{lemma:submatrix-same-eigenvalue}
Suppose the square matrix $\bX_{k+1}\in \real^{(k+1)\times (k+1)}$ has real eigenvalues $\lambda_1, \lambda_2, \ldots, \lambda_{k+1}$. Then, we can construct a $k\times k$ matrix $\bX_{k}$ with eigenvalues $\lambda_2, \lambda_3, \ldots, \lambda_{k+1}$ by 
$$
\bX_{k} = 
\begin{bmatrix}
-\bp_2^\top- \\
-\bp_3^\top- \\
\vdots \\
-\bp_{k+1}^\top-
\end{bmatrix}
\bX_{k+1}
\begin{bmatrix}
\bp_2 & \bp_3 &\ldots &\bp_{k+1}
\end{bmatrix},
$$
where $\bp_1$ is an eigenvector of $\bX_{k+1}$ with norm 1 corresponding to the eigenvalue $\lambda_1$, and $\bp_2, \bp_3, \ldots, \bp_{k+1}$ are any mutually orthonormal vectors that are  orthogonal to $\bp_1$. 
\end{lemma}
\begin{proof}[of Lemma~\ref{lemma:submatrix-same-eigenvalue}]
Let $\bP_{k+1} = [\bp_1, \bp_2, \ldots, \bp_{k+1}]$. We have $\bP_{k+1}^\top\bP_{k+1}=\bI$  and 
$$
\bP_{k+1}^\top \bX_{k+1} \bP_{k+1} =
\begin{bmatrix}
\lambda_1 & \bzero \\
\bzero  & \bX_{k}
\end{bmatrix}. 
$$
For any eigenvalue $\lambda = \{\lambda_2, \lambda_3, \ldots, \lambda_{k+1}\}$, by Lemma~\ref{lemma:determinant-intermezzo}, we have
$$
\begin{aligned}
\det(\bX_{k+1} -\lambda\bI) 
&= \det(\bP_{k+1}^\top (\bX_{k+1}-\lambda\bI)  \bP_{k+1}) 
=\det(\bP_{k+1}^\top \bX_{k+1}\bP_{k+1} - \lambda\bP_{k+1}^\top\bP_{k+1}) \\
&= \det\left(
\begin{bmatrix}
\lambda_1-\lambda & \bzero \\
\bzero & \bX_k - \lambda\bI
\end{bmatrix}
\right)
=(\lambda_1-\lambda)\det(\bX_k-\lambda\bI),
\end{aligned}
$$
where the last equality follows from the fact that if matrix $\bM$ has a block formulation: 
$\bM=
\footnotesize
\begin{bmatrixfoot}
\bE & \bF \\
\bG & \bH 
\end{bmatrixfoot}$, then $\det(\bM) = \det(\bE)\det(\bH-\bG\bE^{-1}\bF)$.
Since $\lambda$ is an eigenvalue of $\bX$ and $\lambda \neq \lambda_1$, then $\det(\bX_{k+1} -\lambda\bI) = (\lambda_1-\lambda)\det(\bX_{k}-\lambda\bI)=0$ means that $\lambda$ is also an eigenvalue of $\bX_{k}$.
\end{proof}

We then prove the existence of the Schur decomposition by induction.
\begin{proof}[of Theorem~\ref{theorem:schur-decomposition}]
We note that the theorem is trivial when $n=1$ by setting $Q=1$ and $U=A$. Suppose the theorem is true when $n=k$ for some $k\geq 1$. If we prove the theorem is also true when $n=k+1$, then we complete the proof.

Suppose for $n=k$, the theorem is true for $\bX_k =\bQ_k \bU_k \bQ_k^\top$. 

Suppose further {$\bP_{k+1}$ contains orthogonal vectors $\bP_{k+1} = [\bp_1, \bp_2, \ldots, \bp_{k+1}]$ as constructed in Lemma~\ref{lemma:submatrix-same-eigenvalue}, where $\bp_1$ is an eigenvector of $\bX_{k+1}$ corresponding to the eigenvalue $\lambda_1$, and its norm is 1; and $\bp_2, \ldots, \bp_{k+1}$ are orthonormal to $\bp_1$}. Let the other $k$ eigenvalues of $\bX_{k+1}$ be $\lambda_2, \lambda_3, \ldots, \lambda_{k+1}$. Since we assume the theorem is true for $n=k$, we can find a matrix $\bX_{k}$ with eigenvalues $\lambda_2, \lambda_3, \ldots, \lambda_{k+1}$. So we have the following property by Lemma~\ref{lemma:submatrix-same-eigenvalue}:
$$
\bP_{k+1}^\top \bX_{k+1} \bP_{k+1} = 
\begin{bmatrix}
\lambda &\bzero \\
\bzero & \bX_k
\end{bmatrix} \qquad
\text{and} \qquad 
\bX_{k+1} \bP_{k+1} =
\bP_{k+1} 
\begin{bmatrix}
\lambda_1 &\bzero \\
\bzero & \bX_k
\end{bmatrix}.
$$
Let 
$
\bQ_{k+1} \triangleq \bP_{k+1}
\begin{bmatrixfoot}
1 &\bzero \\
\bzero & \bQ_k
\end{bmatrixfoot}.
$
Then, it follows that
$$
\begin{aligned}
\bX_{k+1} \bQ_{k+1} 
&= 
\bX_{k+1}
\bP_{k+1}
\begin{bmatrix}
1 &\bzero \\
\bzero & \bQ_k
\end{bmatrix}
=
\bP_{k+1} 
\begin{bmatrix}
\lambda_1 &\bzero \\
\bzero & \bX_k
\end{bmatrix}
\begin{bmatrix}
1 &\bzero \\
\bzero & \bQ_k
\end{bmatrix} 
=
\bP_{k+1}
\begin{bmatrix}
\lambda_1 & \bzero \\
\bzero & \bX_k\bQ_k
\end{bmatrix}\\
&\stackrel{\dag}{=}
\bP_{k+1}
\begin{bmatrix}
\lambda_1 & \bzero \\
\bzero & \bQ_k \bU_k   
\end{bmatrix}   
=\bP_{k+1}
\begin{bmatrix}
1 &\bzero \\
\bzero & \bQ_k
\end{bmatrix}
\begin{bmatrix}
\lambda_1 &\bzero \\
\bzero & \bU_k
\end{bmatrix}
=\bQ_{k+1}\bU_{k+1}, 
\end{aligned}
$$
where the equality ($\dag$) follows from the  the assumption for $n=k$, and the last equality follows from the fact that $\bU_{k+1} = \begin{bmatrixfoot}
	\lambda_1 &\bzero \\
	\bzero & \bU_k
\end{bmatrixfoot}$.
We then have $\bX_{k+1} = \bQ_{k+1}\bU_{k+1}\bQ_{k+1}^\top$, where $\bU_{k+1}$ is an upper triangular matrix, and $\bQ_{k+1}$ is an orthogonal matrix since $\bP_{k+1}$ and 
$\begin{bmatrixfoot}
1 &\bzero \\
\bzero & \bQ_k
\end{bmatrixfoot}$ are both orthogonal matrices.
This completes the proof.
\end{proof}

\subsection*{Other Forms of the Schur Decomposition}\label{section:other-form-schur-decom}

From the proof of the Schur decomposition, we obtain the upper triangular matrix $\bU_{k+1}$ by appending the eigenvalue $\lambda_1$ to $\bU_k$. From this process, the values on the diagonal are always eigenvalues. Therefore, we can decompose the upper triangular into two parts.

\begin{corollary}[Form 2 of Schur decomposition]\label{corollary:schur-second-form}
Any square matrix $\bX\in \real^{n\times n}$ with real eigenvalues admits the following decomposition:
$$
\bQ^\top\bX\bQ = \bLambda +\bT \qquad \textbf{or} \qquad \bX = \bQ(\bLambda +\bT)\bQ^\top,
$$
where $\bQ$ is an orthogonal matrix, $\bLambda=\diag(\lambda_1, \lambda_2, \ldots, \lambda_n)$ is a diagonal matrix containing the eigenvalues of $\bX$, and $\bT$ is a \textit{strictly upper triangular} matrix (with zeros on the diagonal).
\end{corollary}

A strictly upper triangular matrix is an upper triangular matrix having 0's along the diagonal as well as the lower portion. Another proof for this decomposition is that $\bX$ and $\bU$ (where $\bU = \bQ^\top\bX\bQ$) are similar matrices so that they have the same eigenvalues (Lemma~\ref{lemma:eigenvalue-similar-matrices}). And the eigenvalues of any upper triangular matrices are on the diagonal. 
To see this, for any upper triangular matrix $\bR \in \real^{n\times n}$, where the diagonal values are $r_{ii}$ for all $i\in \{1,2,\ldots,n\}$, we have
$$
\bR \be_i = r_{ii}\be_i,
$$
where $\be_i$ is the $i$-th basis vector in $\real^n$.
So we can decompose $\bU$ into $\bLambda$ and $\bT$.

%A final observation on the second form of the Schur decomposition is shown as follows. From $\bX\bQ = \bQ(\bLambda +\bT)$, it follows that 
%$$
%\bX \bq_k = \lambda_k\bq_k + \sum_{i=1}^{k-1}t_{ik}\bq_i,
%$$
%where $t_{ik}$ is the ($i,k$)-th entry of $\bT$. The form is quite close to the eigenvalue decomposition (see \citet{lu2021numerical}). Nevertheless, the columns become orthonormal bases and the orthonormal bases are correlated in some sense.

\subsection*{Spectral Decomposition}
\begin{theoremHigh}[Spectral decomposition]\label{theorem:spectral_theorem}\index{Spectral decomposition}\index{Spectral theorem}
A real matrix $\bX \in \real^{n\times n}$ is symmetric if and only if there exists an orthogonal matrix $\bQ$ and a diagonal matrix $\bLambda$ such that
\begin{equation*}
\bX = \bQ \bLambda \bQ^\top,
\end{equation*}
where the columns of $\bQ = [\bq_1, \bq_2, \ldots, \bq_n]$ are eigenvectors of $\bX$ and are mutually orthonormal, and the entries of $\bLambda=\diag(\lambda_1, \lambda_2, \ldots, \lambda_n)$ are the corresponding eigenvalues of $\bX$, which are real. Specifically, we have the following properties:
\begin{enumerate}
\item A symmetric matrix has only \textbf{real eigenvalues}.
\item The eigenvectors can be chosen \textbf{orthonormal}.
\item The rank of $\bX$ is the number of nonzero eigenvalues.
\item If the eigenvalues are distinct, the eigenvectors are linearly independent.
\end{enumerate}
\end{theoremHigh}

To prove the existence of the spectral decomposition, we need the following lemmas.
\index{Symmetric matrix}
\begin{lemma}[Symmetric matrix property 1 of 4: real eigenvalues]\label{lemma:real-eigenvalues-spectral}
The eigenvalues of any symmetric matrix are all real. 
\end{lemma}
\begin{proof}[of Lemma~\ref{lemma:real-eigenvalues-spectral}]
Suppose eigenvalue $\lambda$ is a complex number $\lambda=a+ib$, where $a$ and $b$ are real. Its complex conjugate is $\bar{\lambda}=a-ib$. Similarly, we have the complex eigenvector $\bbeta = \bc+i\bd$ and its complex conjugate $\bar{\bbeta}=\bc-i\bd$, where $\bc$ and $\bd$ are real vectors. We then have the following property:
$$
\bX \bbeta = \lambda \bbeta\qquad   \underrightarrow{\text{ leads to }}\qquad  \bX \bar{\bbeta} = \bar{\lambda} \bar{\bbeta}\qquad   \underrightarrow{\text{ transpose to }}\qquad  \bar{\bbeta}^\top \bX =\bar{\lambda} \bar{\bbeta}^\top.
$$
We take the dot product of the first equation with $\bar{\bbeta}$ and the last equation with $\bbeta$:
$$
\bar{\bbeta}^\top \bX \bbeta = \lambda \bar{\bbeta}^\top \bbeta \qquad \text{and } \qquad \bar{\bbeta}^\top \bX \bbeta = \bar{\lambda}\bar{\bbeta}^\top \bbeta.
$$
Then we have the equality $\lambda\bar{\bbeta}^\top \bbeta = \bar{\lambda} \bar{\bbeta}^\top\bbeta$. Since $\bar{\bbeta}^\top\bbeta = (\bc-i\bd)^\top(\bc+i\bd) = \bc^\top\bc+\bd^\top\bd$ is a real number,  the imaginary part of $\lambda$ is zero and $\lambda$ is real.
\end{proof}

\index{Symmetric matrix}
\begin{lemma}[Symmetric matrix property 2 of 4: orthogonal eigenvectors]\label{lemma:orthogonal-eigenvectors}
The eigenvectors  corresponding to distinct eigenvalues of any symmetric matrix are orthogonal so that we can normalize eigenvectors to make them orthonormal since $\bX\bbeta = \lambda \bbeta \underrightarrow{\text{ leads to } } \bX\frac{\bbeta}{\normtwo{\bbeta}} = \lambda \frac{\bbeta}{\normtwo{\bbeta}}$ which corresponds to the same eigenvalue.
\end{lemma}
\begin{proof}[of Lemma~\ref{lemma:orthogonal-eigenvectors}]\index{Orthogonal}
Suppose eigenvalues $\lambda_1$ and $\lambda_2$ correspond to eigenvectors $\bbeta_1$ and $\bbeta_2$, respectively, such that $\bX\bbeta_1=\lambda \bbeta_1$ and $\bX\bbeta_2 = \lambda_2\bbeta_2$. We have the following equality:
$$
\bX\bbeta_1=\lambda_1 \bbeta_1 
\qquad\implies\qquad 
\bbeta_1^\top \bX =\lambda_1 \bbeta_1^\top 
\qquad\implies\qquad
\bbeta_1^\top \bX \bbeta_2 =\lambda_1 \bbeta_1^\top\bbeta_2,
$$
and 
$$
\bX\bbeta_2 = \lambda_2\bbeta_2 
\qquad\implies\qquad
\bbeta_1^\top\bX\bbeta_2 = \lambda_2\bbeta_1^\top\bbeta_2,
$$
which implies $\lambda_1 \bbeta_1^\top\bbeta_2=\lambda_2\bbeta_1^\top\bbeta_2$. Since the eigenvalues $\lambda_1\neq \lambda_2$, the eigenvectors are orthogonal.
\end{proof}

For any matrix multiplication, the rank of the multiplication result is does not exceed the rank of the inputs. However, the symmetric matrix $\bX^\top \bX$ is rather special in that the rank of $\bX^\top \bX$ is equal to that of $\bX$ which will be used in the proof of the singular value decomposition in the sequel. 
\begin{lemma}[Rank of $\bX\bY$]\label{lemma:rankAB}
Let $\bX\in \real^{n\times p}$ and $\bY\in \real^{p\times k}$. Then the matrix multiplication $\bX\bY\in \real^{n\times k}$ has $\rank$($\bX\bY$)$\leq$min($\rank$($\bX$), $\rank$($\bY$)).
\end{lemma}

\begin{proof}[of Lemma~\ref{lemma:rankAB}]
For matrix multiplication $\bX\bY$, we have 
\begin{itemize}
\item  All rows of $\bX\bY$ are combination of the rows of $\bY$, the row space of $\bX\bY$ is a subset of the row space of $\bY$. Thus $\rank$($\bX\bY$)$\leq$$\rank$($\bY$).

\item All columns of $\bX\bY$ are combination of columns of $\bX$, the column space of $\bX\bY$ is a subset of the column space of $\bX$. Thus $\rank$($\bX\bY$)$\leq$$\rank$($\bX$).
\end{itemize}

Therefore we have, $\rank$($\bX\bY$)$\leq$min($\rank$($\bX$), $\rank$($\bY$)).
\end{proof}

For the third property of symmetric matrix, we need the definition of similar matrices and the property about eigenvalues of similar matrices (see Lemma~\ref{lemma:eigenvalue-similar-matrices}).

\index{Symmetric matrix}
\begin{lemma}[Symmetric matrix property 3 of 4: orthonormal eigenvectors for duplicate eigenvalue]\label{lemma:eigen-multiplicity}
If $\bX$ has a duplicate eigenvalue $\lambda_i$ with multiplicity $k\geq 2$, then there exist $k$ orthonormal eigenvectors corresponding to $\lambda_i$.
\end{lemma}
\begin{proof}[of Lemma~\ref{lemma:eigen-multiplicity}]
We note that there is at least one eigenvector $\bbeta_{i1}$ corresponding to $\lambda_i$. And for such an eigenvector $\bbeta_{i1}$, we can always find additional $n-1$ orthonormal vectors $\by_2, \by_3, \ldots, \by_n$ so that $\{\bbeta_{i1}, \by_2, \by_3, \ldots, \by_n\}$ forms an orthonormal basis in $\real^n$. Put the $\by_2, \by_3, \ldots, \by_n$ into matrix $\bY_1$ and $\{\bbeta_{i1}, \by_2, \by_3, \ldots, \by_n\}$ into matrix $\bP_1$:
$$
\bY_1=[\by_2, \by_3, \ldots, \by_n] \qquad \text{and} \qquad \bP_1=[\bbeta_{i1}, \bY_1].
$$
We then have
$$
\bP_1^\top\bX\bP_1 = \begin{bmatrix}
\lambda_i &\bzero \\
\bzero & \bY_1^\top \bX\bY_1
\end{bmatrix} = \begin{bmatrix}
\lambda_i &\bzero \\
\bzero & \bB
\end{bmatrix}. \qquad (\text{Let $\bB=\bY_1^\top \bX\bY_1$})
$$
As a result, $\bX$ and $\bP_1^\top\bX\bP_1$ are similar matrices such that they have the same eigenvalues since $\bP_1$ is nonsingular (even orthogonal here, see Lemma~\ref{lemma:eigenvalue-similar-matrices}). We obtain 
$$
\det(\bP_1^\top\bX\bP_1 - \lambda\bI_n) =~\footnote{By the fact that if matrix $\bM$ has a block formulation: $\bM=\tiny\begin{bmatrix}
\bA & \bB \\
\bC & \bD 
\end{bmatrix}$, then $\det(\bM) = \det(\bA)\det(\bD-\bC\bA^{-1}\bB)$.
}~
(\lambda_i - \lambda )\det(\bY_1^\top \bX\bY_1 - \lambda\bI_{n-1}).
$$
If $\lambda_i$ has  a multiplicity $k\geq 2$, then the term $(\lambda_i-\lambda)$ occurs $k$ times in the polynomial from the determinant $\det(\bP_1^\top\bX\bP_1 - \lambda\bI_n)$, i.e., the term occurs $k-1$ times in the polynomial from $\det(\bY_1^\top \bX\bY_1 - \lambda\bI_{n-1})$. In other words, $\det(\bY_1^\top \bX\bY_1 - \lambda_i\bI_{n-1})=0$ and $\lambda_i$ is an eigenvalue of $\bY_1^\top \bX\bY_1$. 
%$$
%\begin{aligned}
%	\bP_2^\top\bY_1^\top \bX\bY_1 \bP_2 &= 
%	\begin{bmatrix}
%		\lambda_i &\bzero \\
%		\bzero    & \bY_2^\top\bY_1^\top \bX\bY_1\bY_2
%	\end{bmatrix}_{(n-1)\times(n-1)} \\
%\bY_1^\top \bX\bY_1  &= 
%\bP_2
%\begin{bmatrix}
%	\lambda_i &\bzero \\
%	\bzero    & \bY_2^\top\bY_1^\top \bX\bY_1\bY_2
%\end{bmatrix}\bP_2^\top.
%\end{aligned}
%$$ 
%where $\bY_2 \in \real^{(n-1)\times(n-2)}$. Write back to $\bP_1^\top\bX\bP_1$, we obtain
%$$
%\begin{aligned}
%\bP_1^\top\bX\bP_1 &= \begin{bmatrix}
%	\lambda_i &\bzero \\
%	\bzero & \bP_2
%	\begin{bmatrix}
%		\lambda_i &\bzero \\
%		\bzero    & \bY_2^\top\bY_1^\top \bX\bY_1\bY_2
%	\end{bmatrix}\bP_2^\top
%\end{bmatrix}\\
%&=
%\begin{bmatrix}
%	\lambda_i &0&\bzero \\
%	\bzero & \lambda_i\bbeta_{i2}&
%\bP_2	\begin{bmatrix}
%		 \bzero \\
%		\bY_2^\top\bY_1^\top \bX\bY_1\bY_2
%	\end{bmatrix}\bP_2^\top
%\end{bmatrix},
%\end{aligned}
%$$
%where the first 

Let $\bB=\bY_1^\top \bX\bY_1$. Since $\det(\bB-\lambda_i\bI_{n-1})=0$, the null space of $\bB-\lambda_i\bI_{n-1}$ is not none. Suppose $(\bB-\lambda_i\bI_{n-1})\bn = \bzero$, i.e., $\bB\bn=\lambda_i\bn$ and $\bn$ is an eigenvector of $\bB$. 

From $
\bP_1^\top\bX\bP_1 = 
\footnotesize
\begin{bmatrix}
\lambda_i &\bzero \\
\bzero & \bB
\end{bmatrix},
$
we have $
\bX\bP_1 
\footnotesize
\begin{bmatrix}
z \\
\bn 
\end{bmatrix} 
\normalsize
= 
\bP_1
\footnotesize
\begin{bmatrix}
\lambda_i &\bzero \\
\bzero & \bB
\end{bmatrix}
\begin{bmatrix}
z \\
\bn 
\end{bmatrix}$, where $z$ is any scalar. From the left side of this equation, we have 
\begin{equation}\label{equation:spectral-pro4-right}
\begin{aligned}
\bX\bP_1 
\begin{bmatrix}
z \\
\bn 
\end{bmatrix} 
&=
\begin{bmatrix}
\lambda_i\bbeta_{i1}, \bX\bY_1
\end{bmatrix}
\begin{bmatrix}
z \\
\bn 
\end{bmatrix} 
=\lambda_iz\bbeta_{i1} + \bX\bY_1\bn.
\end{aligned}
\end{equation}
And from the right side of the equation, we have 
\begin{equation}\label{equation:spectral-pro4-left}
\begin{aligned}
\bP_1
\begin{bmatrix}
\lambda_i &\bzero \\
\bzero & \bB
\end{bmatrix}
\begin{bmatrix}
z \\
\bn 
\end{bmatrix}
&=
\begin{bmatrix}
\bbeta_{i1} & \bY_1
\end{bmatrix}
\begin{bmatrix}
\lambda_i &\bzero \\
\bzero & \bB
\end{bmatrix}
\begin{bmatrix}
z \\
\bn 
\end{bmatrix}
=
\begin{bmatrix}
\lambda_i\bbeta_{i1} & \bY_1\bB 
\end{bmatrix}
\begin{bmatrix}
z \\
\bn 
\end{bmatrix}\\
&= \lambda_i z \bbeta_{i1} + \bY_1\bB \bn 
=\lambda_i z \bbeta_{i1} + \lambda_i \bY_1 \bn, 
\end{aligned}
\end{equation}
where the last equality follows from the fact that $\bB \bn=\lambda_i\bn$.
Combining Equation~\eqref{equation:spectral-pro4-left} and Equation~\eqref{equation:spectral-pro4-right}, we obtain 
$$
\bX\bY_1\bn = \lambda_i\bY_1 \bn,
$$
which means that $\bY_1\bn$ is an eigenvector of $\bX$ corresponding to the eigenvalue $\lambda_i$ (the same eigenvalue corresponding to $\bbeta_{i1}$). Since $\bY_1\bn$ is a combination of $\by_2, \by_3, \ldots, \by_n$, which are orthonormal to $\bbeta_{i1}$, the vector $\bY_1\bn$ can be chosen to be orthonormal to $\bbeta_{i1}$.

To conclude, if we have one eigenvector $\bbeta_{i1}$ corresponding to $\lambda_i$ whose multiplicity is $k\geq 2$, we could construct the second eigenvector by choosing one vector from the null space of $(\bB-\lambda_i\bI_{n-1})$ constructed above. Suppose now, we have constructed the second eigenvector $\bbeta_{i2}$, which is orthonormal to $\bbeta_{i1}$.  
For such eigenvectors $\bbeta_{i1}$ and $\bbeta_{i2}$, we can always find additional $n-2$ orthonormal vectors $\by_3, \by_4, \ldots, \by_n$ so that $\{\bbeta_{i1},\bbeta_{i2}, \by_3, \by_4, \ldots, \by_n\}$ forms an orthonormal basis in $\real^n$. Put the vectors $\by_3, \by_4, \ldots, \by_n$ into matrix $\bY_2$ and $\{\bbeta_{i1},\bbeta_{i2},  \by_3, \by_4, \ldots, \by_n\}$ into matrix $\bP_2$:
$$
\bY_2=[\by_3, \by_4, \ldots, \by_n] \qquad \text{and} \qquad \bP_2=[\bbeta_{i1}, \bbeta_{i2},\bY_1].
$$
We then have
$$
\bP_2^\top\bX\bP_2 = 
\begin{bmatrix}
\lambda_i & 0 &\bzero \\
0& \lambda_i &\bzero \\
\bzero & \bzero & \bY_2^\top \bX\bY_2
\end{bmatrix}
=
\begin{bmatrix}
\lambda_i & 0 &\bzero \\
0& \lambda_i &\bzero \\
\bzero & \bzero & \bC
\end{bmatrix},
$$
where $\bC=\bY_2^\top \bX\bY_2$ such that $\det(\bP_2^\top\bX\bP_2 - \lambda\bI_n) = (\lambda_i-\lambda)^2 \det(\bC - \lambda\bI_{n-2})$. If the multiplicity of $\lambda_i$ is $k\geq 3$, $\det(\bC - \lambda_i\bI_{n-2})=0$ and the null space of $\bC - \lambda_i\bI_{n-2}$ is not none so that we can still find a vector from the null space of $\bC - \lambda_i\bI_{n-2}$ and $\bC\bn = \lambda_i \bn$. Now we can construct a vector $
\footnotesize
\begin{bmatrix}
z_1 \\
z_2\\
\bn
\end{bmatrix}\in \real^n $, where $z_1$ and $z_2$ are any scalar values, such that 
$$
\bX\bP_2\begin{bmatrix}
z_1 \\
z_2\\
\bn
\end{bmatrix} = \bP_2 
\begin{bmatrix}
\lambda_i & 0 &\bzero \\
0& \lambda_i &\bzero \\
\bzero & \bzero & \bC
\end{bmatrix}
\begin{bmatrix}
z_1 \\
z_2\\
\bn
\end{bmatrix}.
$$
Similarly, from the left side of the above equation we will get $\lambda_iz_1\bbeta_{i1} +\lambda_iz_2\bbeta_{i2}+\bX\bY_2\bn$. From the right side of the above equation we will get $\lambda_iz_1\bbeta_{i1} +\lambda_i z_2\bbeta_{i2}+\lambda_i\bY_2\bn$. As a result, 
$$
\bX\bY_2\bn = \lambda_i\bY_2\bn,
$$
where $\bY_2\bn$ is an eigenvector of $\bX$ and is orthogonal to $\bbeta_{i1}$ and $\bbeta_{i2}$. And it is easy to construct the eigenvector to be orthonormal to the first two.

The process can go on, and finally, we will find $k$ orthonormal eigenvectors corresponding to $\lambda_i$. 

Actually, the dimension of the null space of $\bP_1^\top\bX\bP_1 -\lambda_i\bI_n$ is equal to the multiplicity $k$. It also follows that if the multiplicity of $\lambda_i$ is $k$, there cannot be more than $k$ orthogonal eigenvectors corresponding to $\lambda_i$. Otherwise, it will lead to the contradiction that we could find more than $n$ orthogonal eigenvectors.
\end{proof}

The proof of the Spectral Theorem~\ref{theorem:spectral_theorem}  is evident from the lemmas above. Also, we can use Schur decomposition to prove the existence of spectral decomposition (see Theorem~\ref{theorem:schur-decomposition}).
%\begin{proof}[\textbf{of Theorem~\ref{theorem:spectral_theorem}: Existence of Spectral Decomposition}]
%From Schur decomposition in Theorem~\ref{theorem:schur-decomposition}, symmetric matrix $\bX=\bX^\top$ leads to $\bQ\bU\bQ^\top = \bQ\bU^\top\bQ^\top$. Then $\bU$ is a diagonal matrix. And this diagonal matrix actually contains eigenvalues of $\bX$. All the columns of $\bQ$ are eigenvectors of $\bX$. We conclude that all symmetric matrices are diagonalizable even with repeated eigenvalues.
%\end{proof}

\index{Rank}
\index{Symmetric matrix}
\begin{lemma}[Symmetric matrix property 4 of 4: rank of symmetric matrix]\label{lemma:rank-of-symmetric}
If $\bX$ is an $n\times n$ real symmetric matrix, then rank($\bX$) =
the total number of nonzero eigenvalues of $\bX$. 
In particular, $\bX$ has full rank if and only if $\bX$ is nonsingular. Furthermore, $\cspace(\bX)$ is the linear space spanned by the eigenvectors of $\bX$ that correspond to nonzero eigenvalues.
\end{lemma}
\begin{proof}[of Lemma~\ref{lemma:rank-of-symmetric}]
For any symmetric matrix $\bX$, we have $\bX$, in spectral form, as $\bX = \bQ \bLambda\bQ^\top$ and also $\bLambda = \bQ^\top\bX\bQ$. Since we have shown in Lemma~\ref{lemma:rankAB}  that the rank of the matrix multiplication $\rank$($\bX\bY$)$\leq$min($\rank$($\bX$), $\rank$($\bY$)). Therefore, we have
\begin{itemize}
\item From $\bX = \bQ \bLambda\bQ^\top$, we have $\rank(\bX) \leq \rank(\bQ \bLambda) \leq \rank(\bLambda)$;

\item From $\bLambda = \bQ^\top\bX\bQ$, we have $\rank(\bLambda) \leq \rank(\bQ^\top\bX) \leq \rank(\bX)$, 
\end{itemize}

The inequalities above give us a contradiction. And thus $\rank(\bX) = \rank(\bLambda)$, which is the total number of nonzero eigenvalues.

Since $\bX$ is nonsingular if and only if all of its eigenvalues are nonzero, $\bX$ has full rank if and only if $\bX$ is nonsingular.
\end{proof}

\begin{theoremHigh}[Unique power decomposition of PD/PSD matrices]\label{theorem:unique-factor-pd}
Any $n\times n$ PSD (resp. PD) matrix $\bA$ can be \textbf{uniquely} factored as a power of another PSD (resp. PD) matrix $\bB$ such that $\bA =\bB^k$ with $k=\{1,2,\ldots\}$, where $\rank(\bB)=\rank(\bA)$. 
\end{theoremHigh}
\begin{proof}[of Theorem~\ref{theorem:unique-factor-pd}]
We first show the existence of such a positive semidefinite matrix $\bB$ that satisfies $\bA = \bB^k$. 
\paragraph{Existence.} Since $\bA$ is symmetric and positive semidefinite, its spectral decomposition  is given by $\bA = \bQ\bLambda\bQ^\top$, where $\bQ$ is orthogonal and $\bLambda$ is diagonal containing the  eigenvalues of $\bA$. Since eigenvalues of PSD matrices are nonnegative (Theorem~\ref{theorem:eigen_charac}), the $k$-th square root of $\bLambda$ exists. We can define $\bB \triangleq\bQ\bLambda^{1/k}\bQ^\top$ such that $\bA = \bB^k$, where $\bB$ is apparently PSD.

\paragraph{Uniqueness.} 
Suppose the factorization is not unique. Then, there exist two positive definite matrices $\bB_1$ and $\bB_2$ such that
$$
\bA = \bB_1^k = \bB_2^k.
$$
Their spectral decompositions are given by 
$$
\bB_1 = \bQ_1 \bLambda_1\bQ_1^\top \qquad \text{and} \qquad \bB_2 = \bQ_2 \bLambda_2\bQ_2^\top.
$$
We notice that $\bLambda_1^k$ and $\bLambda_2^k$ contain the eigenvalues of $\bA$, and both eigenvalues of $\bB_1$ and $\bB_2$ contained in $\bLambda_1$ and $\bLambda_2$ are nonnegative (since both $\bB_1$ and $\bB_2$ are  PSD). Without loss of generality, we suppose $\bLambda_1=\bLambda_2=\bLambda^{1/k}$, and $\bLambda=\diag(\lambda_1,\lambda_2, \ldots, \lambda_n)$ such that $\lambda_1\geq \lambda_2 \geq \ldots \geq \lambda_n$. Utilizing the equation  $\bB_1^k = \bB_2^k$, we have
$$
\bQ_1 \bLambda \bQ_1^\top = \bQ_2 \bLambda \bQ_2^\top  \quad\implies\quad \bQ_2^\top\bQ_1 \bLambda = \bLambda \bQ_2^\top\bQ_1.
$$
Let $\bZ \triangleq \bQ_2^\top\bQ_1 $ (which is orthogonal), this implies that $\bLambda$ and $\bZ$ commute, and $\bZ$ must be a block-diagonal matrix whose partitioning conforms to the block structure of $\bLambda$. This results in $\bLambda^{1/k} = \bZ\bLambda^{1/k}\bZ^\top$ and
$$
\bB_2 = \bQ_2 \bLambda^{1/k}\bQ_2^\top = \bQ_2 \bQ_2^\top\bQ_1\bLambda^{1/k} \bQ_1^\top\bQ_2 \bQ_2^\top=\bB_1.
$$
Thus, the decomposition is unique. In a similar manner, we can establish the unique decomposition of a PD matrix $\bA =  \bB^k$, where $\bB$ is also PD. 
For a more detailed discussion, see \citet{koeber2006unique, horn2012matrix}, which provides an alternative proof using polynomials.
\end{proof}

\subsection{Singular Value Decomposition (SVD)}\label{section:SVD}

Employing QR decomposition, we factor the matrix into an orthogonal matrix. 
Unlike the factorization into a single orthogonal matrix, singular value decomposition (SVD) yields two orthogonal matrices. We illustrate the result of SVD in the following theorem.

\index{Full and reduced}
\begin{theoremHigh}[Reduced SVD for rectangular matrices]\label{theorem:reduced_svd_rectangular}
Any real $n\times p$ matrix $\bX$ with rank $r$ admits the following decomposition:
$$
\bX = \bU \bSigma \bV^\top,
$$ 
where $\bSigma\in \real^{r\times r}$ is a diagonal matrix $\bSigma=\diag(\sigma_1, \sigma_2 \ldots, \sigma_r)$ with $\sigma_1 \geq \sigma_2 \geq \ldots \geq \sigma_r$, and:
\begin{itemize}
\item The elements $\sigma_i$'s are the nonzero \textit{singular values} of $\bX$, in the meantime, they are the (positive) square roots of the nonzero \textit{eigenvalues} of $\trans{\bX} \bX$ and $ \bX \trans{\bX}$.

\item Columns of $\bU\in \real^{n\times r}$ contain the $r$ eigenvectors of $\bX\bX^\top$ corresponding to the $r$ nonzero eigenvalues of $\bX\bX^\top$. 

\item Columns of $\bV\in \real^{p\times r}$ contain the $r$ eigenvectors of $\bX^\top\bX$ corresponding to the $r$ nonzero eigenvalues of $\bX^\top\bX$. 

\item Moreover, the columns of $\bU$ and $\bV$ are called the \textit{left and right singular vectors} of $\bX$, respectively. 

\item Furthermore, the columns of $\bU$ and $\bV$ are orthonormal (by Spectral Theorem~\ref{theorem:spectral_theorem}). 
\end{itemize}

In particular, we can write out the matrix decomposition $\bX = \bU \bSigma \bV^\top = \sum_{i=1}^r \sigma_i \bu_i \bv_i^\top$, which is a sum of $r$ rank-one matrices.
%$$
%	\bX = \bU \bSigma \bV^\top = \sum_{i=1}^r \sigma_i \bu_i \bv_i^\top, 
%$$
%which is a sum of $r$ rank 1 matrices.
\end{theoremHigh}
If we append additional $n-r$ silent columns that are orthonormal to the $r$ eigenvectors of $\bX\bX^\top$, just like the silent columns in QR decomposition, we will have an orthogonal matrix $\bU\in \real^{n\times n}$. A similar procedure applies to the columns of $\bV$. 
%We then illustrate the full SVD for matrices in Theorem~\ref{theorem:full_svd_rectangular}, where we formulate the difference between reduced and full SVD in the blue text. 
The comparison between the reduced and full SVD is shown in Figure~\ref{fig:svd-comparison}, where white entries are zero, and blue entries are not necessarily zero.

\begin{figure}[h!]
\centering  
\vspace{-0.35cm}  
\subfigtopskip=2pt  
\subfigbottomskip=2pt  
\subfigcapskip=-5pt  
\subfigure[Reduced SVD decomposition.]{\label{fig:svdhalf}
\includegraphics[width=0.47\linewidth]{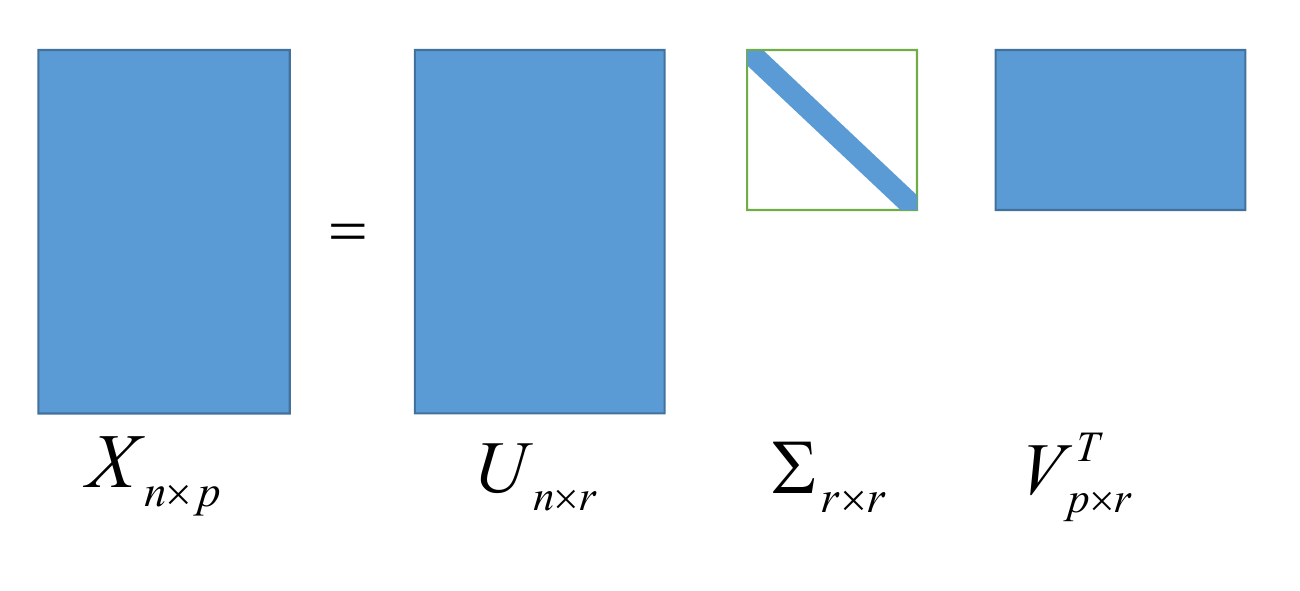}}
\quad 
\subfigure[Full SVD decomposition.]{\label{fig:svdall}
\includegraphics[width=0.47\linewidth]{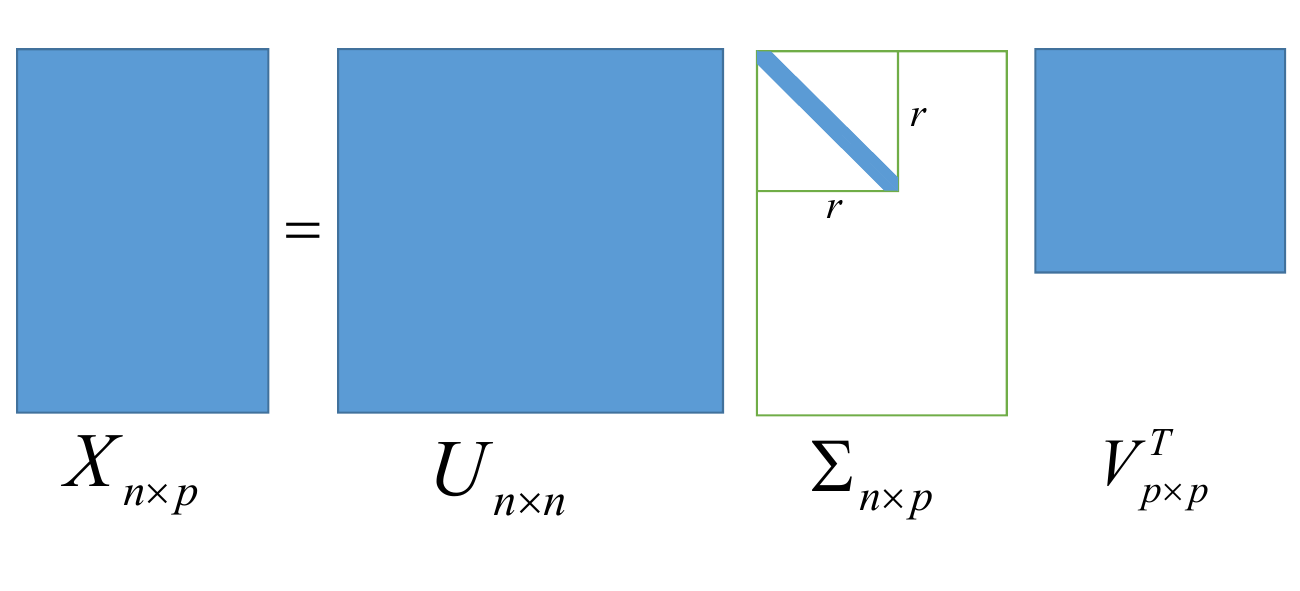}}
\caption{Comparison between the reduced and full SVD. White entries are zero, and blue entries are not necessarily zero}
\label{fig:svd-comparison}
\end{figure}

\subsection*{Existence of the SVD}
To prove the existence of the SVD, we need to use the following lemmas. We may notice that the singular values are the square roots of the eigenvalues of $\bX^\top\bX$. While, negative values do not have square roots such that its eigenvalues must be nonnegative.
\index{Nonnegative eigenvalues}
\begin{lemma}[Nonnegative eigenvalues of $\bX^\top \bX$]\label{lemma:nonneg-eigen-ata}
For any matrix $\bX\in \real^{n\times p}$, $\bX^\top \bX$ has nonnegative eigenvalues.
\end{lemma}
\begin{proof}[of Lemma~\ref{lemma:nonneg-eigen-ata}]
Given an eigenvalue and its corresponding eigenvector $\lambda$ and $\bbeta$ of $\bX^\top \bX$, we have
$$
\bX^\top \bX \bbeta = \lambda \bbeta \leadto \bbeta^\top \bX^\top \bX \bbeta = \lambda \bbeta^\top\bbeta. 
$$
Since $\bbeta^\top \bX^\top \bX \bbeta  = \normtwo{\bX \bbeta}^2 \geq 0$ and $\bbeta^\top\bbeta \geq 0$. It then follows that  $\lambda \geq 0$.
\end{proof}

Since $\bX^\top\bX$ has nonnegative eigenvalues, we then can define the singular value $\sigma\geq 0$ of $\bX$ such that $\sigma^2$ is the eigenvalue of $\bX^\top\bX$, i.e., {$\bX^\top\bX \bv = \sigma^2 \bv$}. This is essential for establishing the existence of SVD.

We have shown in Lemma~\ref{lemma:rankAB} that $\rank$($\bX\bY$)$\leq$min($\rank$($\bX$), $\rank$($\bY$)).
However, the symmetric matrix $\bX^\top \bX$ is rather special in that the rank of $\bX^\top \bX$ is equal to $\rank(\bX)$. And the proof is provided in the following lemma.

%To prove the theorem above, we must show that the matrix $\bX^\top\bX$ is invertible. Since we assume $\bX$ has full rank and $n\geq p$, the product $\bX^\top\bX \in \real^{p\times p}$ will be invertible if it has rank $p$, which turns out to be exactly the same as the rank of $\bX$. 
%We formalize this in the following lemma.
\begin{lemma}[Rank of $\bX^\top \bX$]\label{lemma:rank-of-ata-x} 
For any matrix $\bX$, $\bX^\top \bX$ and $\bX$ have same rank.
Similarly, $\bX \bX^{\top}$ and $\bX$ have same rank:
$$
\rank(\bX) = \rank(\bX^\top\bX) =\rank(\bX\bX^\top).
$$
\end{lemma}
\begin{proof}[of Lemma~\ref{lemma:rank-of-ata-x}]
Let $\boldeta\in \nspace(\bX)$, that is, a vector in the null space of $\bX$ such that $\bX\bbeta=\bzero$. Then,
$$
\bX\boldeta  = \bzero 
\qquad \implies \qquad
\bX^\top\bX \boldeta =\bzero, 
$$
which means $\boldeta\in \nspace(\bX) \implies \boldeta \in \nspace(\bX^\top \bX)$. Therefore, $\nspace(\bX) \in \nspace(\bX^\top\bX)$. 

Conversely, suppose $\boldeta \in \nspace(\bX^\top\bX)$, we have 
$$
\bX^\top \bX\boldeta = \bzero\implies \boldeta^\top \bX^\top \bX\boldeta = 0\implies \normtwo{\bX\boldeta}^2 = 0 \implies \bX\boldeta=\bzero.
$$
This shows that $\boldeta\in \nspace(\bX^\top \bX) \implies \boldeta\in \nspace(\bX)$. Therefore, $\nspace(\bX^\top\bX) \in\nspace(\bX) $. Combining both inclusions, we conclude: $\nspace(\bX) = \nspace(\bX^\top\bX)$ and $\dim(\nspace(\bX)) = \dim(\nspace(\bX^\top\bX))$. Applying the fundamental theorem of linear algebra in Theorem~\ref{theorem:fundamental-linear-algebra}, we conclude that $\bX^\top \bX$ and $\bX$ have the same rank.
\end{proof}
By applying the same reasoning to $\bX^\top$, we can also show that $\bX\bX^\top$ and $\bX$ share the same rank. 
The ordinary least squares estimate is a result of this conclusion.

In the form of SVD, we claimed the matrix $\bX$ is a sum of $r$ rank-one matrices, where $r$ is the number of nonzero singular values. And the number of nonzero singular values is actually equal to the rank of the matrix.

\index{Rank}
\begin{lemma}\label{lemma:rank-equal-singular}
The number of nonzero singular values of a matrix $\bX$ equals the rank of $\bX$.
\end{lemma}
\begin{proof}[of Lemma~\ref{lemma:rank-equal-singular}]
The rank of any symmetric matrix (here $\bX^\top\bX$) equals the number of nonzero eigenvalues (with repetitions) by Lemma~\ref{lemma:rank-of-symmetric}. So the number of nonzero singular values equals the rank of $\bX^\top \bX$. By Lemma~\ref{lemma:rank-of-ata-x}, $\bX^\top \bX$ and $\bX$ have the same rank, so the number of nonzero singular values equals  the rank of $\bX$.
\end{proof}

We are now ready to prove the existence of SVD.
\begin{proof}[of Theorem~\ref{theorem:reduced_svd_rectangular}]
Since $\bX^\top \bX$ is a symmetric matrix, by Spectral Theorem~\ref{theorem:spectral_theorem}  and Lemma~\ref{lemma:nonneg-eigen-ata}, there exists an orthogonal matrix $\bV$ such that
$$
{\bX^\top \bX = \bV \bSigma^2 \bV^\top},
$$
where $\bSigma$ is a diagonal matrix containing the singular values of $\bX$, i.e., $\bSigma^2$ contains the eigenvalues of $\bX^\top \bX$.
Specifically, $\bSigma=\diag(\sigma_1, \sigma_2 \ldots, \sigma_r)$ and $\{\sigma_1^2, \sigma_2^2, \ldots, \sigma_r^2\}$ are the nonzero eigenvalues of $\bX^\top \bX$ with $r$ being the rank of $\bX$. I.e., $\{\sigma_1, \ldots, \sigma_r\}$ are the singular values of $\bX$. In this case, $\bV\in \real^{p\times r}$.
Start from {$\bX^\top\bX \bv_i = \sigma_i^2 \bv_i$}, $\forall\, i \in \{1, 2, \ldots, r\}$, i.e., the eigenvector $\bv_i$ of $\bX^\top\bX$ corresponding to $\sigma_i^2$:

1. Multiply both sides by $\bv_i^\top$:
$$
\bv_i^\top\bX^\top\bX \bv_i = \sigma_i^2 \bv_i^\top \bv_i \leadto \norm{\bX\bv_i}^2 = \sigma_i^2 \leadto \norm{\bX\bv_i}=\sigma_i.
$$

2. Multiply both sides by $\bX$:
$$
\bX\bX^\top\bX \bv_i = \sigma_i^2 \bX \bv_i \leadtosmall \bX\bX^\top \frac{\bX \bv_i }{\sigma_i}= \sigma_i^2 \frac{\bX \bv_i }{\sigma_i} \leadtosmall \bX\bX^\top \bu_i = \sigma_i^2 \bu_i,
$$
where we notice this form can find the eigenvector of $\bX\bX^\top$ corresponding to $\sigma_i^2$, and the eigenvector is $\bX \bv_i$. Since the length of $\bX \bv_i$ is $\sigma_i$, we then define $\bu_i \triangleq \frac{\bX \bv_i }{\sigma_i}$ with norm 1.
These $\bu_i$'s are orthogonal because $(\bX\bv_i)^\top(\bX\bv_j)=\bv_i^\top\bX^\top\bX\bv_j=\sigma_j^2 \bv_i^\top\bv_j=0$. That is,
$$
{\bX \bX^\top = \bU \bSigma^2 \bU^\top}.
$$
Since $\bX\bv_i = \sigma_i\bu_i$, we have 
$$
[\bX\bv_1, \bX\bv_2, \ldots, \bX\bv_r] = [ \sigma_1\bu_1,  \sigma_2\bu_2, \ldots,  \sigma_r\bu_r]\leadto
\bX\bV = \bU\bSigma,
$$
Since $\bV\bV^\top \neq \bI$, we cannot obtain the reduced SVD directly. 
Suppose we append the semi-orthogonal matrix  $\bV$ into an orthogonal matrix $\widetilde{\bV}=[\bV, \bV_2]$, and append the semi-orthogonal matrix $\bU$ into an orthogonal matrix $\widetilde{\bU}=[\bU, \bU_2]$. We then obtain
$$
\bA \widetilde{\bV}= \widetilde{\bU} \widetilde{\bSigma} ,\gap \text{where}\gap 
\widetilde{\bSigma} 
= \begin{bmatrix}
	\bSigma & \bzero  \\
	\bzero &\bzero
\end{bmatrix}
\quad\implies\quad
\bA = \widetilde{\bU} \widetilde{\bSigma}\widetilde{\bV}^\top,
$$
i.e., the full SVD (since $\widetilde{\bV}\widetilde{\bV}^\top=\bI$).
Simplifying the product, we get:
$
\bA =\bU \bSigma \bV^\top + \bU_2 \cdot \bzero \cdot \bV_2^\top = \bU \bSigma \bV^\top,
$
i.e., the reduced SVD,
which completes the proof.
\end{proof}

\subsection*{Eckart-Young-Mirsky Theorem}
Suppose we want to approximate the rank-$r$ matrix $\bX\in \real^{n\times p}$ by a rank-$k$ matrix $\bY$ ($k<r$). The approximation is measured using the Frobenius norm (Definition~\ref{definition:frobernius-in-svd}):\index{Low-rank approximation}
$$
\bY = \mathop{\arg\min}_{\bY} \, \norm{\bX - \bY}_F.
$$
Then  we can recover the optimal rank-$k$ approximation using the following theorem \citep{stewart1993early}.

\index{Truncated SVD}
\begin{theoremHigh}[Eckart-Young-Mirsky theorem w.r.t. Frobenius norm\index{Eckart-Young-Mirsky theorem}]\label{theorem:young-theorem_frob}
Given a matrix $\bX\in \real^{n\times p}$ and $1\leq k\leq \rank(\bX)=r$, and let $\bX_k$ be the \textit{truncated SVD} (TSVD) of $\bX$ with the largest $k$ terms, i.e., $\bX_k = \sum_{i=1}^{k} \sigma_i\bu_i\bv_i^\top$ from the SVD of $\bX=\sum_{i=1}^{r} \sigma_i\bu_i\bv_i^\top$ by zeroing out the $r-k$ trailing singular values of $\bX$. Then $\bX_k$ is the optimal rank-$k$ approximation to $\bX$ in terms of the Frobenius norm, satisfying $\normf{\bX-\bX_k}^2 = \sum_{i\geq k+1}\sigma_i^2$.
\end{theoremHigh}

\subsection*{Four Orthonormal Bases in SVD}\label{appendix:property-svd}

\index{Fundamental theorem of linear algebra}
For any matrix, we have the following property:
\begin{itemize}
	\item $\nspace(\bX)$ is the orthogonal complement of the row space $\cspace(\bX^\top)$ in $\real^p$: $\dim(\nspace(\bX))+\dim(\cspace(\bX^\top))=p$;
	
	\item $\nspace(\bX^\top)$ is the orthogonal complement of the column space $\cspace(\bX)$ in $\real^n$: $\dim(\nspace(\bX^\top))+\dim(\cspace(\bX))=n$;
\end{itemize}
This is called the fundamental theorem of linear algebra and is also known as the rank-nullity theorem  (Theorem~\ref{theorem:fundamental-linear-algebra}).
In specific, the construction of SVD yields a set of orthonormal bases for the four subspaces in the fundamental theorem of linear algebra. 
To show this, we require the following fact.
\begin{lemma}[Subspace of $\bX^\top \bX$ and $\bX\bX^\top$]\label{lemma:rank-of-ttt}
Let $\bX\in \real^{n\times p}$ be given. Then,
\begin{itemize}
\item The column space of $\bX^\top \bX$ is identical  to the column space of $\bX^\top$ (i.e., row space of $\bX$): $\cspace(\bX^\top\bX)=\cspace(\bX^\top)$; this also shows $\nspace(\bX^\top\bX)=\nspace(\bX)$ by fundamental theorem of linear algebra in Theorem~\ref{theorem:fundamental-linear-algebra}.
\item The column space of $\bX\bX^\top$ is identical  to the column space of $\bX$: $\cspace(\bX\bX^\top)=\cspace(\bX)$; again, this also shows
$\nspace(\bX\bX^\top)=\nspace(\bX^\top)$.
\end{itemize}
\end{lemma}
\begin{proof}[of Lemma~\ref{lemma:rank-of-ttt}]
Let $\bbeta\in \nspace(\bX)$, we have 
$
\bX\bbeta  = \bzero \implies \bX^\top\bX \bbeta =\bzero, 
$
i.e., $\bbeta\in \nspace(\bX) \implies \bbeta \in \nspace(\bX^\top \bX)$. Therefore, $\nspace(\bX) \subseteq \nspace(\bX^\top\bX)$. 
Furthermore, let $\bbeta \in \nspace(\bX^\top\bX)$, we have 
$$
\bX^\top \bX\bbeta = \bzero\implies \bbeta^\top \bX^\top \bX\bbeta = 0\implies \normtwo{\bX\bbeta}^2 = 0 \implies \bX\bbeta=\bzero, 
$$
i.e., $\bbeta\in \nspace(\bX^\top \bX) \implies \bbeta\in \nspace(\bX)$. Therefore, $\nspace(\bX^\top\bX) \subseteq\nspace(\bX) $. 
As a result, by ``sandwiching," it follows that  
$
\nspace(\bX) = \nspace(\bX^\top\bX).
$
According to the fundamental theorem of linear algebra in Theorem~\ref{theorem:fundamental-linear-algebra}, we have 
$$
\cspace(\bX^\top)=\cspace(\bX^\top\bX).
$$
%Alternatively, Since, 
%\begin{enumerate}
%\item $\bX^\top\bX$ is symmetric, then the row space of $\bX^\top\bX$ equals the column space of $\bX^\top\bX$. 
%\item All rows of $\bX^\top\bX$ are linear combinations of rows of $\bX$, so the row space of $\bX^\top\bX$ $\subseteq$ row space of $\bX$, i.e., $\cspace(\bX^\top\bX) \subseteq \cspace(\bX^\top)$. 
%\item Since $\rank(\bX^\top\bX) = \rank(\bX)$ by Lemma~\ref{lemma:rank-of-ata-x}, we can conclude that:
%The row space of $\bX^\top\bX$ = the column space of $\bX^\top\bX$ =  the row space of $\bX$, i.e., $\cspace(\bX^\top\bX) = \cspace(\bX^\top)$. 
%\end{enumerate}
Applying the same process to $\bX^\top$ leads to the second part of the lemma.
\end{proof}

\index{Orthonormal basis}
\index{Fundamental theorem}
\begin{theoremHigh}[Four orthonormal bases in SVD]\label{theorem:svd-four-orthonormal-Basis}
Given the full SVD of matrix $\bX = \bU \bSigma \bV^\top$, where $\bU=[\bu_1, \bu_2, \ldots,\bu_n]$ and $\bV=[\bv_1, \bv_2, \ldots, \bv_p]$ are the column partitions of $\bU$ and $\bV$, respectively. Then, we have the following property:
\begin{itemize}
\item $\{\bv_1, \bv_2, \ldots, \bv_r\} $ is an orthonormal basis of the row space,  $\cspace(\bX^\top)$;

\item $\{\bv_{r+1},\bv_{r+2}, \ldots, \bv_p\}$ is an orthonormal basis of the null space, $\nspace(\bX)$;

\item $\{\bu_1,\bu_2, \ldots,\bu_r\}$ is an orthonormal basis of the column space, $\cspace(\bX)$;

\item $\{\bu_{r+1}, \bu_{r+2},\ldots,\bu_n\}$ is an orthonormal basis of the left null space, $\nspace(\bX^\top)$. 
\end{itemize}
\end{theoremHigh}
\begin{proof}[of Theorem~\ref{theorem:svd-four-orthonormal-Basis}]
From Lemma~\ref{lemma:rank-of-symmetric}, for symmetric matrix $\bX^\top\bX$, its column space $\cspace(\bX^\top\bX)$ is spanned by the eigenvectors. 
Therefore, the set $\{\bv_1,\bv_2 \ldots, \bv_r\}$ forms an orthonormal basis for $\cspace(\bX^\top\bX)$.
Thus, $\{\bv_1, \bv_2,\ldots, \bv_r\}$ also serves as an orthonormal basis for $\cspace(\bX^\top)$ by Lemma~\ref{lemma:rank-of-ttt}. 

Furthermore, the space spanned by $\{\bv_{r+1}, \bv_{r+2},\ldots, \bv_n\}$ is an orthogonal complement to the space spanned by $\{\bv_1,\bv_2, \ldots, \bv_r\}$. Hence, $\{\bv_{r+1},\bv_{r+2}, \ldots, \bv_n\}$ constitutes an orthonormal basis for $\nspace(\bX)$. 

Applying this process to $\bX\bX^\top$ proves the remaining claims in the lemma. 
Alternatively, we can see that $\{\bu_1,\bu_2, \ldots,\bu_r\}$ forms a basis for the column space of $\bX$ by Lemma~\ref{lemma:column-basis-from-row-basis}~\footnote{As a recap, for any matrix $\bX$, let $\{\br_1, \br_2, \ldots, \br_r\}$ be a set of vectors in $\real^p$, which forms a basis for the row space, then $\{\bX\br_1, \bX\br_2, \ldots, \bX\br_r\}$ is a basis for the column space of $\bX$.}, since $\bu_i = \frac{\bX\bv_i}{\sigma_i},\, \forall\, i \in\{1, 2, \ldots, r\}$. 
\end{proof}
The relationship among the four subspaces is demonstrated in Figure~\ref{fig:lafundamental3-SVD}, where $\bX$ maps each row basis vector $\bv_i$ into the column basis vector $\bu_i$ by $\sigma_i\bu_i=\bX\bv_i$ for all $i\in \{1, 2, \ldots, r\}$.

\begin{figure}[h!]
\centering
\includegraphics[width=0.95\textwidth]{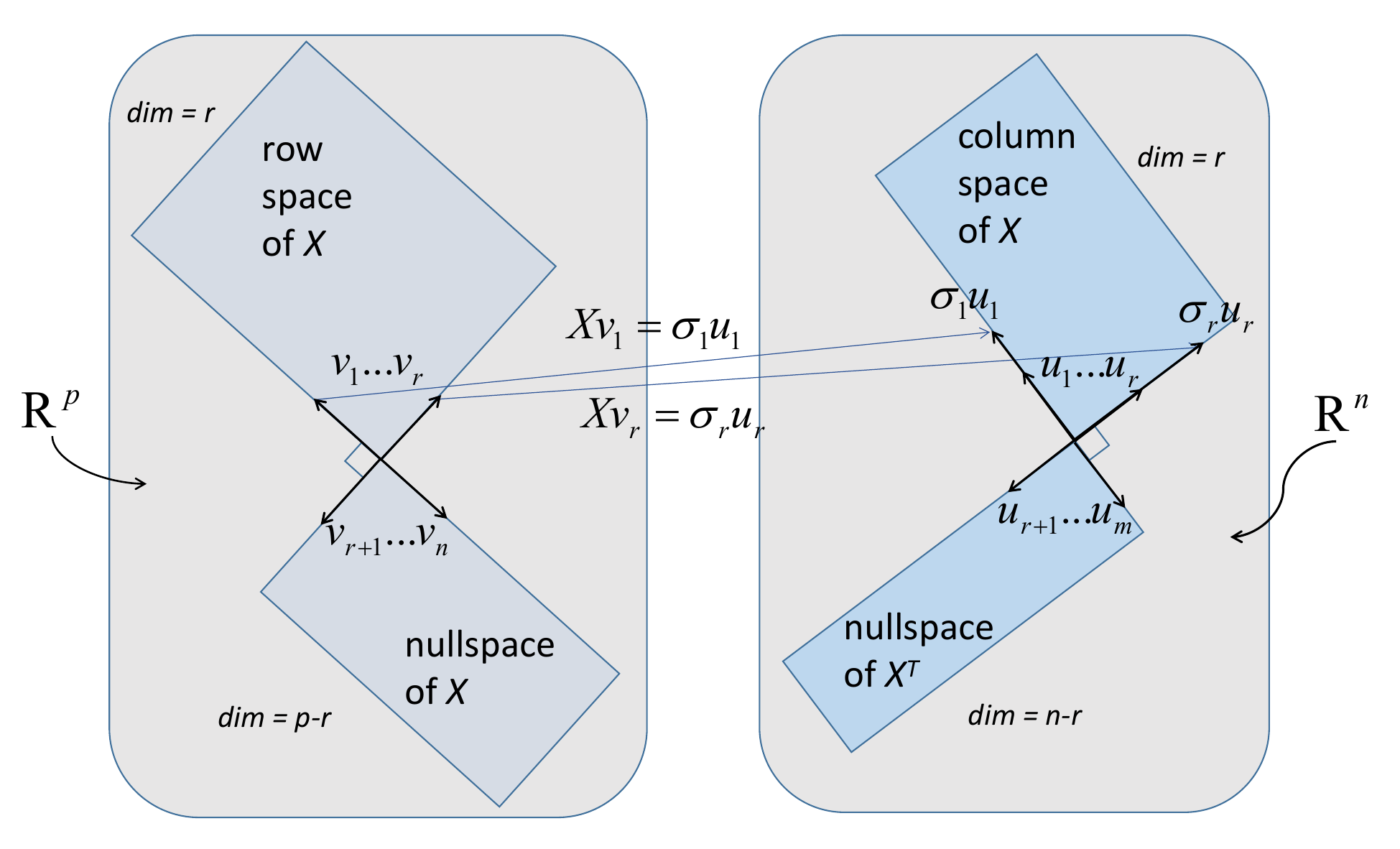}
\caption{Orthonormal bases that diagonalize $\bX$ from SVD. $\{\bv_1, \bv_2, \ldots, \bv_r\} $ is an orthonormal basis of $\cspace(\bX^\top)$, and $\{\bu_1,\bu_2, \ldots,\bu_r\}$ is an orthonormal basis of $\cspace(\bX)$. Connection between the row space basis and column space basis: $\bX$ transfers the row basis $\bv_i$ into the column basis $\bu_i$ by $\sigma_i\bu_i=\bX\bv_i$ for all $i\in \{1, 2, \ldots, r\}$.}
\label{fig:lafundamental3-SVD}
\end{figure}

\section{Pseudo-Inverse}\label{section:pseudo-inverse}
If the matrix $\bX$ is nonsingular, the solution to the linear system $\by = \bX\bbeta$ can be directly obtained by taking the inverse of $\bX$, yielding $\widehatbbeta = \bX^{-1}\by$. 
However, when $\bX$ is not square or is singular, the inverse does not exist. In such cases, we can still define a generalized inverse known as the \textit{pseudo-inverse}, represented as a $p\times n$ matrix denoted by $\bX^+$.

Before discussing the pseudo-inverse in detail, we will briefly introduce related concepts such as one-sided inverses, generalized inverses, and reflexive generalized inverses.
That said, readers who are already familiar with these ideas may choose to skip ahead without losing the overall understanding of the pseudo-inverse.

\subsection{One-Sided Inverse}
We begin by providing a formal definition of the one-sided inverse:
\begin{definition}[One-sided inverse]\index{One-sided inverse}\label{definition:one_side_inverse}
For any matrix $\bX\in \real^{n\times p}$, a matrix $\bX_L^{-1}$ is called a \textit{left inverse} of $\bX$ if it satisfies the condition:
$$
\bX_L^{-1} \bX = \bI_p.
$$
In such cases, the matrix $\bX$ is said to be \textit{left-invertible}.
Similarly, a matrix $\bX_R^{-1}$ is referred to as a \textit{right inverse} of $\bX$ if the following holds:
$$
\bX \bX_R^{-1}= \bI_n.~\footnote{The superscript $-1$ in $\bX_L^{-1}$ and $\bX_R^{-1}$ signifies the one-sided inverse of $\bX$ and should not be interpreted as the inverse of $\bX_L$ or $\bX_R$.}
$$
Here, $\bX$ is said to be \textit{right-invertible}.
\end{definition}

\begin{lemma}[One-sided invertible]\label{theorem:one-sided-invertible}
For any matrix $\bX\in \real^{n\times p}$, the following hold:
\begin{itemize}
\item  $\bX$ is left-invertible if and only if $\bX$ has full column rank (which implies $n\geq p$);
\item  $\bX$ is right-invertible if and only if $\bX$ has full row rank (which implies $n\leq p$).
\end{itemize}
\end{lemma}
\begin{proof}[of Lemma~\ref{theorem:one-sided-invertible}]
Suppose $\bX$ has full column rank. Then the matrix $\bX^\top \bX \in \real^{p\times p}$ attains full rank (by Lemma~\ref{lemma:rank-of-ata-x}). Therefore, $(\bX^\top \bX)^{-1}(\bX^\top \bX) = \bI_p$. This implies that \textcolor{black}{$(\bX^\top \bX)^{-1}\bX^\top$ acts as  a left inverse of $\bX$}. 

Conversely, suppose $\bX$ is left-invertible with $\bX_L^{-1} \bX = \bI_p$. Since all rows of $\bX_L^{-1} \bX$ are  combinations of the rows of $\bX$, meaning the row space of $\bX_L^{-1} \bX$ is a subset of the row space of $\bX$. We then have $\rank(\bX) \geq \rank(\bX_L^{-1} \bX) =\rank(\bI_p) = p$, indicating $\rank(\bX)=p$, and $\bX$ has full column rank. 

Similarly, we can show that $\bX$ is right-invertible if and only if $\bX$ has full row rank, and \textcolor{black}{$\bX^\top (\bX\bX^\top)^{-1}$ serves as a right inverse of $\bX$}.
\end{proof}

From the  proof above, we see that $(\bX^\top \bX)^{-1}\bX^\top$ is  a specific left inverse of $\bX$ when it has full column rank. 
Similarly, $\bX^\top (\bX\bX^\top)^{-1}$ acts as a specific right inverse of $\bX$ when it has full row rank.
However, obtaining the inverse of a nonsingular $p\times p$ matrix involves a complex process, requiring $2p^3$ floating-point operations (flops) \citep{lu2021numerical}.
In our case, finding the inverses of $\bX^\top\bX$ and $\bX\bX^\top$ would require $2p^3$ and $2n^3$ flops, respectively.  
A more straightforward approach to acquire a one-sided inverse involves using elementary operations.

Assume that $\bX\in \real^{n\times p}$ has full column rank. 
We can apply \textit{row elementary operations}, represented by a matrix $\bE\in \real^{n\times n}$, to the augmented matrix $[\bX, \bI_n]$, resulting in
\begin{equation}\label{equation:onesided-1}
\bE \begin{bmatrix}
\bX & \bI_n
\end{bmatrix}=
\begin{bmatrix}
\bI_p & \bG \\
\bzero & \bZ
\end{bmatrix},
\end{equation}
where $\bG \in \real^{p\times n}$, $\bI_n$ is an $n\times n$ identity matrix, $\bI_p$ is a $p\times p$ identity matrix, and $\bZ$ is an $(n-p)\times n$ matrix.
Then, it can be easily verified that $\bG\bX = \bI_p$, establishing $\bG$ as a left inverse of $\bX$.

Similarly, consider $\bX\in \real^{n\times p}$ with full row rank. By applying \textit{column elementary operations}, denoted by $\bE \in \real^{p\times p}$, to the matrix $[\bX^\top, \bI_p]^\top$, we obtain
\begin{equation}\label{equation:onesided-2}
\begin{bmatrix}
\bX \\
\bI_p
\end{bmatrix} \bE=
\begin{bmatrix}
\bI_n & \bzero\\
\bG & \bZ
\end{bmatrix},
\end{equation}
where $\bZ$ is a $p\times (p-n)$ matrix.
Then, $\bG \in \real^{p\times n}$ is a right inverse of $\bX$.

More generally, the following two propositions provide the methods for discovering more left inverses or right inverses of a matrix.
\begin{proposition}[Finding left inverse]\index{Left inverse}
Suppose $\bX\in \real^{n\times p}$ is left-invertible ($n\geq p$). Then, 
$$
\bX_L^{-1} = [(\bX_1^{-1} -\bY \bX_2 \bX_1^{-1}) ,  \bY]\bE,
$$
is a left inverse of $\bX$, where $\bY\in \real^{p\times (n-p)}$ can be any matrix, and $\bE\bX = \begin{bmatrixfoot}
\bX_1 \\
\bX_2
\end{bmatrixfoot}$ is the row elementary transformation of $\bX$ such that $\bX_1 \in \real^{p\times p}$ is invertible (since $\bX$ has full column rank $p$) and $\bE\in \real^{n\times n}$.
\end{proposition}
We can verify that $\bG$ in Equation~\eqref{equation:onesided-1} is a specific left inverse of $\bX$ by setting $\bY=\bzero$. Since $\bE = \begin{bmatrixfoot}
\bG \\
*
\end{bmatrixfoot}$, $\bX_1 = \bI_p$, and $\bX_2 = \bzero$, we have 
$$
\bX_L^{-1} = [(\bX_1^{-1} -\bY \bX_2 \bX_1^{-1}), \bY]\bE = \bG + \bY\bZ = \bG,
$$
where the last equality follows from the assumption that $\bY=\bzero$.

\begin{proposition}[Finding right inverse]\index{Right inverse}
Suppose $\bX\in \real^{n\times p}$ is right-invertible ($n\leq p$). Then,
$$
\bX_R^{-1} = \bE\begin{bmatrix}
(\bX_1^{-1}-\bX_1^{-1}\bX_2\bY) \\
\bY
\end{bmatrix},
$$
is a right inverse of $\bX$, where $\bY\in \real^{(p-n)\times n}$ can be any matrix, and $\bX\bE = [\bX_1, \bX_2 ]$ is the column elementary transformation of $\bX$ such that $\bX_1 \in \real^{n\times n}$ is invertible (since $\bX$ has full row rank $n$) and $\bE\in \real^{p\times p}$.
\end{proposition}
Similarly, we can verify that $\bG$ in Equation~\eqref{equation:onesided-2} is a specific right inverse of $\bX$ by setting $\bY=\bzero$. Since $\bE = [\bG, \bZ]$, $\bX_1=\bI_n$, and $\bX_2=\bzero$, we have 
$$
\bX_R^{-1} = \bE\begin{bmatrix}
(\bX_1^{-1}-\bX_1^{-1}\bX_2\bY) \\
\bY
\end{bmatrix} = \bG+\bZ\bY = \bG,
$$
where again the last equality holds because $\bY=\bzero$.

\subsection{Generalized Inverse (g-inverse)}
We mentioned previously that if the matrix $\bX$ is nonsingular,  the linear system $\by = \bX\bbeta$ can be easily solved using the inverse of $\bX$, resulting in $\widehatbbeta = \bX^{-1}\by$. However, for an $n\times p$ matrix $\bX$, the inverse does not exist if $\bX$ is neither square nor nonsingular. 

Nevertheless, when $\by$ lies in the column space of $\bX$, we can still determine the solution to the linear system. The association between the solution $\widehatbbeta$ and the target vector $\by$ is expressed by the \textit{generalized inverse} (or \textit{inner inverse or $\{1\}$-inverse}) of  $\bX$: $\widehatbbeta = \bX^-\by$ \citep{nashed1973generalized}.
\begin{definition}[Generalized inverse]\index{Generalized inverse}
Let $\bX \in \real^{n\times p}$ be a matrix of  rank $r$, where $r \leq p \leq n$. 
Then, a \textit{generalized inverse (g-inverse)} $\bX^- \in\real^{p\times n}$ of $\bX$ is a matrix that satisfies
$$
(C1) \qquad  \bX\bX^-\bX = \bX,
$$
or equivalently, 
$$
(C1.1) \qquad \bX\bX^-\by = \by
$$
for any vector $\by \in \cspace(\bX)$.
\end{definition}
To demonstrate  the equivalence between $(C1)$ and $(C1.1)$, that is, we want to show $\bX$ satisfies $(C1)$ if and only if it satisfies $(C1.1)$. For any $\by \in \cspace(\bX)$, a $\bbeta\in\real^p$ exists such that $\bX\bbeta = \by$. If $\bX$ and $\bX^-$ satisfy $(C1)$, then
$$
\bX\bX^-\bX \bbeta = \bX\bbeta 
\qquad \implies\qquad 
\bX\bX^-\by = \by,
$$
indicating that $\bX$ and $\bX^-$ also satisfy $(C1.1)$. Conversely, if $\bX$ and $\bX^-$ satisfy $(C1.1)$, then:
$$
\bX\bX^-\by = \by 
\qquad \implies\qquad  
\bX\bX^-\bX \bbeta = \bX\bbeta ,
$$
which implies $\bX$ and $\bX^-$ also satisfy $(C1)$.

Multiply on the left of $(C1)$ by $\bX^-$ and utilize the definition of the projection matrix in Definition~\ref{definition:projection-matrix} (i.e., an idempotent matrix), we obtain $\bX^-\bX\bX^-\bX = \bX^-\bX$ such that $ \bX^-\bX$ is idempotent, which implies $\bX^-\bX$ is a projection matrix (not necessarily  an orthogonal projection, i.e., a symmetric and idempotent matrix).
\begin{lemma}[Projection matrix from generalized inverse]\label{lemma:idempotent-of-ginverse}
For any matrix $\bX$  and any of its generalized inverse $\bX^-$, $\bX^-\bX$ is a projection matrix but not necessarily an orthogonal projection matrix. Same claim can be applied to $\bX\bX^-$ as well.
\end{lemma}

\begin{lemma}[Rank of generalized inverse]\label{proposition:rank-of-ginverse}
For any matrix $\bX \in \real^{n\times p}$ and any of its generalized inverse $\bX^- \in\real^{p\times n}$, the following inequality holds: 
$$
\rank(\bX^-) \geq \rank(\bX).
$$
Specifically, we also have $\rank(\bX)=\rank(\bX\bX^-)=\rank(\bX^-\bX)$.
\end{lemma}
\begin{proof}[of Lemma~\ref{proposition:rank-of-ginverse}]
From condition $(C1)$, we know that $\rank(\bX) = \rank( \bX\bX^-\bX)$. Moreover,
$$
\rank( \bX\bX^-\bX) \leq \rank(\bX\bX^-)\leq \rank(\bX^-),
$$
where the first inequality follows because the columns of $\bX\bX^-\bX$ are linear combinations of the columns of $\bX\bX^-$, and the second inequality follows because the rows of $\bX\bX^-$ are  linear combinations of the rows of $\bX^-$.

For the second part, we have 
$$
\rank(\bX) \geq \rank(\bX\bX^-) \geq \rank(\bX\bX^-\bX),
$$
where the first inequality follows because the columns of $\bX\bX^-$ are linear combinations of the columns of $\bX$, and the second inequality follows because the columns of $\bX\bX^-\bX$ are linear combinations of the columns of $\bX\bX^-$. 
From $(C1)$ again, $\rank(\bX) = \rank( \bX\bX^-\bX)$, which implies by ``sandwiching" that
$$
\rank(\bX) = \rank(\bX\bX^-) = \rank(\bX\bX^-\bX).
$$ 
Similarly, we also have 
$$
\rank(\bX) \geq \rank(\bX^-\bX) \geq \rank(\bX\bX^-\bX),
$$
where the first inequality follows because the rows of $\bX^-\bX$ are linear combinations of the rows of $\bX$, and the second inequality follows because the rows of $\bX\bX^-\bX$ are linear combinations of the rows of $\bX^-\bX$. By ``sandwiching" again, we have 
$$
\rank(\bX) = \rank(\bX^-\bX) = \rank(\bX\bX^-\bX),
$$ 
which completes the proof.
\end{proof}

In Lemma~\ref{theorem:one-sided-invertible}, we demonstrated  that a left inverse exists if and only if $\bX$ has full column rank, and a right inverse exists if and only if $\bX$ has full row rank. 
However, these full-rank conditions are not required for the existence of a generalized inverse. When such full-rank conditions are satisfied, the following property holds:
\begin{lemma}[Full-rank generalized inverse]\label{Lemma:fullrank-ginverse}
Given any matrix $\bX \in \real^{n\times p}$ and its generalized inverse $\bX^- \in\real^{p\times n}$, the following statements hold:
\begin{enumerate}[(i)]
\item  $\bX$ has full column rank if and only if $\bX^-\bX = \bI_p$;
\item  $\bX$ has full row rank if and only if $\bX\bX^- = \bI_n$.
\end{enumerate}
\end{lemma}
\begin{proof}[of Lemma~\ref{Lemma:fullrank-ginverse}]
For (i), suppose $\bX$ has full column rank, and we have shown in Lemma~\ref{proposition:rank-of-ginverse} that $\rank(\bX)=\rank(\bX\bX^-)=\rank(\bX^-\bX)$. Thus, $\rank(\bX^-\bX)=\rank(\bX)=p$, and $\bX^-\bX\in \real^{p\times p}$ is nonsingular. We have 
$$
\bI_p= (\bX^-\bX)(\bX^-\bX)^{-1} = \bX^-(\bX\bX^-\bX)(\bX^-\bX)^{-1} = \bX^-\bX.
$$
Conversely, suppose $\bX^-\bX = \bI_p$, which implies $\rank(\bX^-\bX)=p$. From $\rank(\bX^-\bX)=\rank(\bX)$, we have $\rank(\bX)=p$ such that $\bX$ has full column rank.

Similarly, we can show $\bX$ has full row rank if and only if $\bX\bX^- = \bI_n$.
\end{proof}

\begin{lemma}[Constructing generalized inverse]\label{Lemma:construct-ginverse}
Let $\bX \in \real^{n\times p}$ be any matrix with a  generalized inverse $\bX^- \in\real^{p\times n}$. 
Then there exists a $p\times n$ matrix $\bA$ such that the matrix 
\begin{equation}\label{equation:constructing-ginverse}
\overline{\bX} = \bX^-+ \bA-\bX^-\bX\bA\bX\bX^-
\end{equation}
is also a generalized inverse of $\bX$. 
Moreover, for any generalized inverse $\overline{\bX}$ of $\bX$, there exists some matrix $\bA$ so that Equation~\eqref{equation:constructing-ginverse} is satisfied.
\end{lemma}
\begin{proof}[of Lemma~\ref{Lemma:construct-ginverse}]
Write out the equation
$$
\begin{aligned}
\bX\overline{\bX}\bX &= \bX(\bX^-+ \bA-\bX^-\bX\bA\bX\bX^-)\bX = \bX\bX^-\bX + \bX\bA\bX  -(\bX\bX^-\bX)\bA(\bX\bX^-\bX)\\
&= \bX\bX^-\bX + \bX\bA\bX  -\bX\bA\bX = \bX.
\end{aligned}
$$
Therefore, $\overline{\bX}$ satisfies condition $(C1)$ and is indeed a generalized inverse of $\bX$. 

Now suppose that $\bB$ is any generalized inverse of $\bX$, and define $\bA \triangleq \bB-\bX^-$. Recall that $\bX\bB\bX = \bX$, we have 
$$
\begin{aligned}
\bX^-+ \bA-\bX^-\bX\bA\bX\bX^- &= \bX^-+ (\bB-\bX^-)-\bX^-\bX(\bB-\bX^-)\bX\bX^- \\
&= \bB - \bX^-(\bX\bB\bX)\bX^- + \bX^-(\bX\bX^-\bX)\bX^-\\
&= \bB - \bX^-\bX\bX^- + \bX^-\bX\bX^- 
= \bB,
\end{aligned}
$$
which implies that the matrix $\bA$ can be constructed for any generalized inverse $\bB$ of $\bX$.
\end{proof}

We end up this section by providing more properties of the generalized inverse.
\begin{lemma}[Generalized inverse properties]\label{proposition:ginverse-properties}
Let $\bX \in \real^{n\times p}$ be any matrix  and let  $\bX^- \in\real^{p\times n}$ be a generalized inverse of $\bX$. 
Then the following properties hold:
\begin{enumerate}[(i)]
\item  $(\bX^\top)^- = (\bX^-)^\top$, i.e., $(\bX^-)^\top$ is a generalized inverse of $\bX^\top $;
\item  For any scalar $a\neq0$, $\frac{1}{a} \bX^-$ is a generalized inverse of $a\bX$;
\item  If $\bA\in \real^{n\times n}$ and $\bB\in \real^{p\times p}$ are both nonsingular, then $\bB^{-1}\bX^- \bA^{-1}$ is a generalized inverse of $\bA\bX\bB$;
\item  $\cspace(\bX\bX^-) = \cspace(\bX)$ and $\nspace(\bX^-\bX) = \nspace(\bX)$.
\end{enumerate}
\end{lemma}
\begin{proof}[of Lemma~\ref{proposition:ginverse-properties}]
\textbf{(i).} From condition $(C1)$, $\bX\bX^-\bX = \bX$, we have $\bX^\top (\bX^-)^\top\bX^\top =\bX^\top $ such that $(\bX^-)^\top$ is a generalized inverse of $\bX^\top $.
\paragraph{(ii).} It can be easily verified that $(a\bX) (\frac{1}{a} \bX^-) (a\bX) = (a\bX)$ such that $\frac{1}{a} \bX^-$ is a generalized inverse of $a\bX$ for any $a\neq 0$.

\paragraph{(iii).} We realize that $(\bA\bX\bB) (\bB^{-1}\bX^- \bA^{-1})(\bA\bX\bB)= \bA\bX\bX^- \bX\bB=\bA\bX\bB$, which implies $\bB^{-1}\bX^- \bA^{-1}$ is a generalized inverse of $\bA\bX\bB$.

\paragraph{(iv).} Note that the columns of $\bX\bX^-$ are linear combinations of the columns of $\bX$, so: $\cspace(\bX\bX^-) \subseteq \cspace(\bX)$. From Lemma~\ref{proposition:rank-of-ginverse}, we know $\rank(\bX)=\rank(\bX\bX^-)$, then $\cspace(\bX\bX^-) = \cspace(\bX)$. Similarly, we could prove $\nspace(\bX^-\bX) = \nspace(\bX)$.
This completes the proof.
\end{proof}

\subsection{Reflexive Generalized Inverse (rg-inverse)}

\begin{figure}[h!]
\centering
\includegraphics[width=0.8\textwidth]{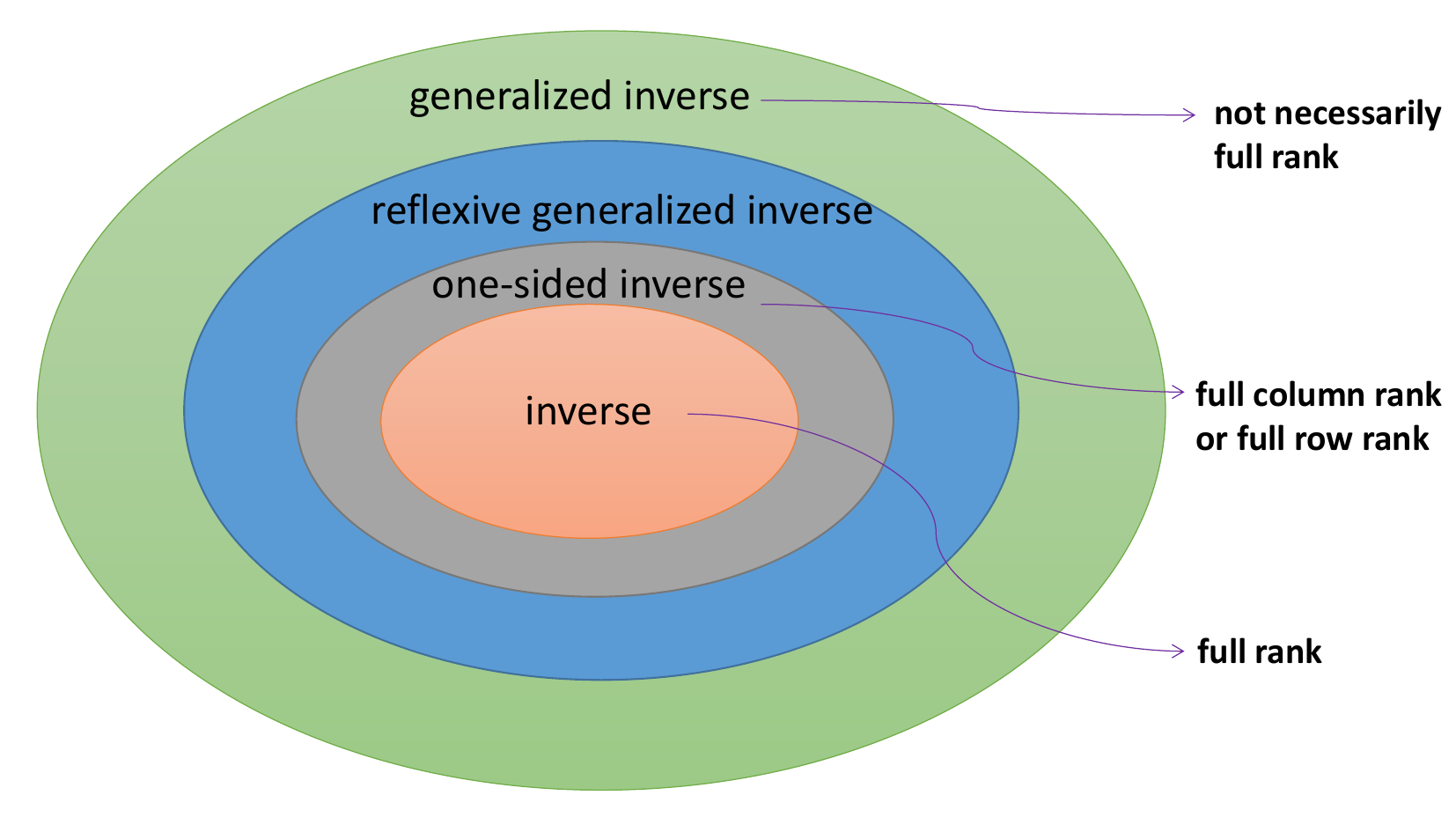}
\caption{Relationship of different inverses: inverse, one-sided inverse, reflexive generalized inverse, and generalized inverse.}
\label{fig:pseudo-inverse-comparison}
\end{figure}
\begin{definition}[Reflexive Generalized Inverse]\index{Reflexive generalized inverse}
Let $\bX \in \real^{n\times p}$ be a matrix of rank $r$, where $r \leq \min\{p, n\}$. Then, a \textit{reflexive generalized inverse (rg-inverse)} $\bX_r^- \in\real^{p\times n}$ of $\bX$ is a matrix that satisfies the conditions:
$$
(C1) \qquad  \bX\bX_r^-\bX = \bX, \qquad \text{$\bX_r^-$ is an inner inverse}
$$
and
$$
(C2) \qquad \bX_r^- \bX \bX_r^- = \bX_r^-. \qquad \text{$\bX_r^-$ is an outer inverse}
$$ 
In other words, $\bX_r^-$ is a g-inverse of $\bX$, and at the same time, $\bX$ is a g-inverse of $\bX_r^-$.
\end{definition}
Note that $(C1)$ shows that $\bX_r^-$ is a \textit{generalized inverse} (or \textit{inner inverse}, \textit{$\{1\}$-inverse}) or $\bX$; while $(C2)$ shows that  $\bX$ is a generalized inverse of $\bX_r^-$. 
Equivalently, we say that $\bX_r^-$ is an \textit{outer inverse} or a \textit{$\{2\}$-inverse} of $\bX$ by $(C2)$ \citep{nashed1973generalized, bjorck2024numerical}.

Suppose the matrix $\bX$ is of rank $r$. Then, it can be factored as $\bX = \bE_1 
\begin{bmatrixfoot}
\bI_r & \bzero\\
\bzero & \bzero
\end{bmatrixfoot}\bE_2$, where $\bE_1\in \real^{n\times n}$ and $\bE_2\in \real^{p\times p}$ are elementary transformations on $\bX$. Then, we can construct such a reflexive generalized inverse of $\bX$ as 
\begin{equation}
\bX_r^- = \bE_2^{-1} \begin{bmatrix}
\bI_r & \bA \\
\bB  & \bB\bA
\end{bmatrix}\bE_1^{-1},
\end{equation}
where $\bA\in \real^{r\times (n-r)}$, $\bB\in \real^{(p-r)\times r}$ can be any arbitrary matrices.
This shows that the reflexive generalized inverse is generally \textbf{not unique}. 
This construction of the reflexive generalized inverse also shows that a reflexive generalized inverse exists for any matrix $\bX$. 
Therefore, compared to one-sided inverses---which only exist under strict full-rank conditions---the reflexive generalized inverse is a more general concept.
\begin{lemma}[Reflexive generalized inverse from g-inverse]\label{lemma:generalized-for-reflexive-two}
Let $\bX \in \real^{n\times p}$ be any matrix, and suppose $\bA$ and $\bB$ are both generalized inverses of $\bX$. 
Define the matrix:
$$
\bZ \triangleq \bA\bX\bB.
$$
Then, $\bZ$ is a reflexive generalized inverse of $\bX$.
\end{lemma}
It can be easily verified that $\bX\bZ\bX = \bX$ and $\bZ\bX\bZ = \bZ$ for the  lemma above.

\begin{lemma}[Reflexive generalized inverse from g-inverse]\label{lemma:generalized-for-reflexive-two22}
For any matrix $\bX \in \real^{n\times p}$,  the following two matrices $\bA$ and $\bB$ are two reflexive generalized inverses of $\bX$:
$$
\begin{aligned}
\bA &=(\bX^\top\bX)^-\bX^\top, \\
\bB &=\bX^\top (\bX\bX^\top)^-,
\end{aligned}
$$
where $(\bX^\top\bX)^-$ is a g-inverse of $(\bX^\top\bX)$, and $(\bX\bX^\top)^-$ is a g-inverse of $(\bX\bX^\top)$.
\end{lemma}
\begin{proof}[of Lemma~\ref{lemma:generalized-for-reflexive-two22}] 
By Lemma~\ref{lemma:rank-of-ttt}, we have  $\cspace(\bX^\top \bX) = \cspace(\bX^\top)$ and $\nspace(\bX^\top \bX)=\nspace(\bX)$.
Then there exists a set of vectors $\bz_1, \bz_2, \ldots, \bz_n \in \real^p$ such that the $i$-th column of $\bX^\top$ can be expressed as $\bX^\top\bX\bz_i$. 
Let $\bZ\triangleq[\bz_1, \bz_2, \ldots, \bz_n]$, so we have:
$
\bX^\top = \bX^\top\bX\bZ.
$
Using this identity and the condition $(C1)$ of g-inverse, compute:
\begin{equation}\label{equation:reflex-eq1}
\begin{aligned}
\bX\bA\bX
& = (\bX^\top\bX\bZ)^\top  (\bX^\top\bX)^-\bX^\top \bX \\
&=\bZ^\top \bX^\top\bX (\bX^\top\bX)^-\bX^\top\bX= \bZ^\top \bX^\top\bX=\bX.
\end{aligned}
\end{equation}
%By condition $(C1.1)$ of g-inverse, we have $\bX^\top\bX (\bX^\top\bX)^- \by = \by $ for any $\by\in\cspace(\bX^\top\bX)\equiv\cspace(\bX^\top)$. This implies  $\bX^\top\bX (\bX^\top\bX)^-\bX^\top\bX = \bX^\top\bX$ and 
%\begin{equation}\label{equation:reflex-eq1}
%\bX\bA\bX = (\bX^\top\bX\bZ)^\top  (\bX^\top\bX)^-\bX^\top \bX = \bZ^\top \bX^\top\bX=\bX.
%\end{equation}
Now, let's examine $\bA\bX\bA$:
$$
\bA\bX\bA = (\bX^\top\bX)^-\bX^\top \bX(\bX^\top\bX)^-\bX^\top.
$$
The same argument applies to $\bX^\top \bX(\bX^\top\bX)^-\bX^\top = \bX^\top$. Thus,
\begin{equation}\label{equation:reflex-eq2}
\bA\bX\bA = (\bX^\top\bX)^-\bX^\top = \bA.
\end{equation}
Combining Equation~\eqref{equation:reflex-eq1} and Equation~\eqref{equation:reflex-eq2}, we conclude that $\bA$ is a reflexive generalized inverse of $\bX$. Similarly, we can show $\bB$ is a reflexive generalized inverse of $\bX$ as well.
\end{proof}

From the definition, it is clear that a reflexive generalized inverse is a special type of generalized inverse. However, under certain conditions, these two types of inverses are equivalent.
\begin{lemma}[Reflexive generalized inverse in g-inverse]\label{lemma:reflexive-from-ginverse}
Let $\bX \in \real^{n\times p}$ be any matrix, and let  $\bX^-\in \real^{p\times n}$ be a generalized inverse of $\bX$. 
Then, $\bX^-$ is a reflexive generalized inverse of $\bX$ if and only if $\rank(\bX)=\rank(\bX^-)$.
\end{lemma}
\begin{proof}[of Lemma~\ref{lemma:reflexive-from-ginverse}]
Suppose $\bX^-$ is a g-inverse of $\bX$, , so by definition: $\bX\bX^-\bX=\bX$. Suppose further, $\bX^-$ is also a rg-inverse, then $\bX^-\bX\bX^- = \bX^-$. We have 
$$
\begin{aligned}
\rank(\bX) &= \rank(\bX\bX^-\bX)\leq \rank(\bX^-) = \rank(\bX^-\bX\bX^-)\leq \rank(\bX)\\
\end{aligned}
$$
where the two inequalities follows from Lemma~\ref{proposition:rank-of-ginverse}. This implies $\rank(\bX)=\rank(\bX^-)$. 

Conversely, suppose $\bX^-$ is a g-inverse of $\bX$, then $\bX\bX^-\bX=\bX$. And suppose further that $\rank(\bX)=\rank(\bX^-)$, we have
$$
\rank(\bX) = \rank(\bX\bX^-\bX) \leq \rank(\bX^-\bX) \leq \rank(\bX^-) = \rank(\bX),
$$
where the first inequality follows because the rows of $\bX\bX^-\bX$ are combinations of the rows of $\bX^-\bX$, and the second inequality follows because the columns of $\bX^-\bX$ are combinations of the columns of $\bX^-$. 
This implies $\rank(\bX^-\bX) = \rank(\bX^-)$ and $\cspace(\bX^-\bX)=\cspace(\bX^-)$. Then, there exists a set of vectors $\balpha_1, \balpha_2, \ldots, \balpha_n \in \real^p$ such that the column-$i$ of $\bX^-$ can be expressed as $\bX^-\bX\balpha_i$. That is, for $\bA=[\balpha_1, \balpha_2, \ldots, \balpha_n]$, we have
$$
\bX^- = \bX^-\bX\bA.
$$
We realize again that $\bX=\bX\bX^-\bX$, then
$$
\bX=\bX\bX^-\bX = \bX(\bX^-\bX\bA)\bX = \bX\bA\bX,
$$
where the last equality follows form condition $(C1.1)$; and thus, $\bA$ is a g-inverse of $\bX$. From Lemma~\ref{lemma:generalized-for-reflexive-two}, $\bX^-=\bX^-\bX\bA$ is a rg-inverse of $\bX$, which completes the proof.
\end{proof}

To summarize this section, we now provide insight into the rank relationships in reflexive generalized inverses.
\begin{proposition}[Rank of reflexive generalized inverse]
Consider a matrix $\bX \in \real^{n\times p}$ and its generalized inverse $\bX_r^- \in\real^{p\times n}$.
Utilizing the result in Lemma~\ref{lemma:reflexive-from-ginverse} and the result from the rank of g-inverses in Lemma~\ref{proposition:rank-of-ginverse}, we have
$$
\rank(\bX_r^-)=\rank(\bX)=\rank(\bX\bX_r^-)=\rank(\bX_r^-\bX).
$$
\end{proposition}

\begin{lemma}[Reflexive generalized inverse properties]\label{Lemma:refle-ginverse-properties}
Given a matrix $\bX \in \real^{n\times p}$ and its reflexive generalized inverse $\bX_r^- \in\real^{p\times n}$, the following properties hold:
\begin{enumerate}
\item  $\cspace(\bX\bX_r^-) = \cspace(\bX)$ and $\nspace(\bX_r^-\bX) = \nspace(\bX)$.
\item  $\cspace(\bX_r^-\bX) = \cspace(\bX_r^-)$ and $\nspace(\bX\bX_r^-) = \nspace(\bX_r^-)$. 
\end{enumerate}
\end{lemma}
\begin{proof}[of Lemma~\ref{Lemma:refle-ginverse-properties}]
Suppose $\bX^-$ is a g-inverse of $\bX$,
we show in Lemma~\ref{proposition:ginverse-properties}  that $\cspace(\bX\bX^-) = \cspace(\bX)$ and $\nspace(\bX^-\bX) = \nspace(\bX)$. Since $\bX^-_r$ is a g-inverse of $\bX$, and $\bX$ is a g-inverse of $\bX^-_r$, we complete the proof.
\end{proof}

\index{Pseudo-inverse}
\subsection{Pseudo-Inverse}\label{section:subsec_pseudo_inv}
As previously mentioned, for a matrix $\bX\in \real^{n\times p}$, we can find its pseudo-inverse, a $p\times n$ matrix denoted by $\bX^+$.
In simple terms, when $\bX$ multiplies a vector $\bbeta$ that lies in its row space, this produces $\bX\bbeta$ in the column space (see Figure~\ref{fig:lafundamental-ls}). 
Both of these spaces have the same dimension $r$, i.e., the rank of $\bX$. 
When restricted to these subspaces, $\bX$ behaves like an invertible matrix, and $\bX^+$ acts as its inverse. Specifically:
\begin{itemize}
\item If $\bbeta$ is in the row space of $\bX$, then  $\bX^+\bX\bbeta = \bbeta$. 
\item If $\by$ is in the column space of $\bX$, then $\bX\bX^+\by = \by$  (see Figure~\ref{fig:lafundamental5-pseudo}).
\end{itemize}

The null space of $\bX^+$ coincides with  the null space of $\bX^\top$. It contains the vectors $\by$ in $\real^n$ with $\bX^\top\by = \bzero$. Those vectors $\by$ are orthogonal to every vector $\bX\bbeta$ in the column space of $\bX$. We delay the proof of this property in Theorem~\ref{theorem:pseudo-four-basis-space}.% after introducing the pseudo-inverse via SVD.

More formally, the \textit{pseudo-inverse},  also known as the \textit{Moore-Penrose pseudo-inverse}, $\bX^+$, is defined by the unique $p\times n$ matrix satisfying the following four conditions, often referred to as the \textit{Penrose conditions} \citep{penrose1955generalized}:
\begin{equation}\label{equation:pseudi-four-equations}
\boxed{
\begin{aligned}
&(C1) \qquad  \bX\bX^+\bX &=& \bX \qquad &(\bX^+\text{ is a g-inverse of }\bX) \\
&(C2) \qquad  \bX^+\bX\bX^+ &=&\bX^+ \qquad &(\bX\text{ is a g-inverse of }\bX^+)\\
&(C3) \qquad  (\bX\bX^+)^\top &=&\bX\bX^+ \qquad & (\text{$ \bX \bX^+ $ is symmetric})\\
&(C4) \qquad  	(\bX^+\bX)^\top &=& \bX^+\bX\qquad & (\text{$ \bX \bX^+ $ is symmetric})\\
\end{aligned}
}
\end{equation}
Although we  mostly work with real matrices, when $\bX\in\sF^{n\times p}$ (where $\sF$ denotes either $\real$ for real or $\complex$ for complex numbers), the pseudo-inverse satisfies similar conditions using the conjugate transpose:
\begin{equation}\label{equation:pseudi-four-equations_comp}
\boxed{
\begin{aligned}
	&(C1) \qquad  \bX\bX^+\bX &=& \bX \qquad &(\bX^+\text{ is a g-inverse of }\bX) \\
	&(C2) \qquad  \bX^+\bX\bX^+ &=&\bX^+ \qquad &(\bX\text{ is a g-inverse of }\bX^+)\\
	&(C3) \qquad  (\bX\bX^+)^* &=&\bX\bX^+ \qquad & (\text{$ \bX \bX^+ $ is Hermitian})\\
	&(C4) \qquad  	(\bX^+\bX)^* &=& \bX^+\bX \qquad &  (\text{$ \bX^+ \bX $ is Hermitian})
\end{aligned}
}
\end{equation}

In Lemma~\ref{lemma:idempotent-of-ginverse}, we claimed that $\bX\bX^+$ and $\bX^+\bX$ are idempotent if $\bX^+$ is a g-inverse of $\bX$, and thus they are both projection matrices (Definition~\ref{definition:projection-matrix}, i.e., an idempotent matrix). Since $\bX^+$ is the pseudo-inverse of $\bX$~\footnote{We speak of ``the pseudo-inverse" rather than ``a pseudo-inverse" since the pseudo-inverse is unique as we will prove shortly.}, by conditions $(C3)$ and $(C4)$, these projections are symmetric such that they not just general projections, but  orthogonal projections (Lemma~\ref{lemma:symmetric-projection-matrix}, symmetric idempotent matrices are called orthogonal projectors).

The existence of the pseudo-inverse for any matrix is supported by the CR or rank decomposition of the matrix.
\begin{lemma}[Existence of pseudo-inverse]\label{lemma:existence-of-pseudo-inverse}
Every matrix $\bX$ has a pseudo-inverse.
\end{lemma}
\begin{proof}[of Lemma~\ref{lemma:existence-of-pseudo-inverse}] Given the CR decomposition (Theorem~\ref{theorem:cr-decomposition}) or the rank decomposition (Theorem~\ref{theorem:rank-decomposition}) of $\bX=\bC\bR\in\real^{n\times p}$, let 
$$
\bX^+ = \bR^+\bC^+ = \bR^\top (\bR\bR^\top)^{-1} (\bC^\top\bC)^{-1}\bC^\top,
$$
where $\bR^+=\bR^\top (\bR\bR^\top)^{-1}$ and $\bC^+ =(\bC^\top\bC)^{-1}\bC^\top$.~\footnote{It can be easily verified that $\bR^+$ is the pseudo-inverse of $\bR$ and $\bC^+$ is the pseudo-inverse of $\bC$.} Notably, $\bR\bR^\top$ and $\bC^\top\bC$ are invertible since $\bC\in \real^{n\times r}$ and $\bR\in \real^{r\times p}$ have full rank $r$ due to the properties of the CR decomposition.

Now, we verify the Penrose conditions:
$$
\begin{aligned}
&(C1) \qquad \bX\bX^+\bX &=& \bC\bR\left(\bR^\top (\bR\bR^\top)^{-1} (\bC^\top\bC)^{-1}\bC^\top\right)\bC\bR = \bC\bR = \bX, \\
&(C2) \qquad \bX^+\bX\bX^+ &=&\left(\bR^\top (\bR\bR^\top)^{-1} (\bC^\top\bC)^{-1}\bC^\top\right) \bC\bR\left(\bR^\top (\bR\bR^\top)^{-1} (\bC^\top\bC)^{-1}\bC^\top\right)\\
&&=&\bR^\top(\bR\bR^\top)^{-1}(\bC^\top\bC)^{-1}\bC^\top = \bX^+,\\
&(C3) \qquad (\bX\bX^+)^\top &=& \bC(\bC^\top\bC)^{-1} (\bR\bR^\top)^{-1} \bR\bR^\top\bC^\top = \bC(\bC^\top\bC)^{-1} \bC^\top\\
&&=&\bC\bR\bR^\top (\bR\bR^\top)^{-1} (\bC^\top\bC)^{-1}\bC^\top = \bX\bX^+,\\
&(C4) \qquad (\bX^+\bX)^\top &=& \bR^\top\bC^\top \bC(\bC^\top\bC)^{-1} (\bR\bR^\top)^{-1} \bR= \bR^\top (\bR\bR^\top)^{-1} \bR\\
&&=& \bR^\top (\bR\bR^\top)^{-1} (\bC^\top\bC)^{-1}\bC^\top\bC\bR = \bX^+\bX.
\end{aligned}
$$
Since all four Penrose conditions are satisfied, $\bX^+$ is indeed the pseudo-inverse of $\bX$, proving that the pseudo-inverse exists for any $\bX$.
\end{proof}

\index{Uniqueness}
\begin{lemma}[Uniqueness of pseudo-inverse]\label{lemma:uniqueness-of-pseudo-inverse}
Every matrix $\bX$ has a unique pseudo-inverse.
\end{lemma}
\begin{proof}[of Lemma~\ref{lemma:uniqueness-of-pseudo-inverse}]
Suppose $\bX_1^+$ and $\bX_2^+$ are two pseudo-inverses of $\bX$. Then
$$
\begin{aligned}
\bX_1^+ &= \bX_1^+\bX\bX_1^+ = \bX_1^+(\bX\bX_2^+\bX)\bX_1^+ = \bX_1^+ (\bX\bX_2^+)(\bX\bX_1^+) \qquad &(\text{by $(C2), (C1)$})\\
&=\bX_1^+ (\bX\bX_2^+)^\top(\bX\bX_1^+)^\top =  \bX_1^+\bX_2^{+\top}\bX^\top \bX_1^{+\top}\bX^\top \qquad &(\text{by $(C3)$})\\
&=\bX_1^+\bX_2^{+\top}(\bX \bX_1^{+}\bX)^\top= \bX_1^+\bX_2^{+\top}\bX^\top \qquad &(\text{by $(C1)$})\\
&= \bX_1^+(\bX \bX_2^{+})^\top =\bX_1^+ \bX \bX_2^{+} = \bX_1^+ (\bX \bX_2^+\bX) \bX_2^{+}  \qquad &(\text{by $(C3),(C1)$})\\
&= (\bX_1^+ \bX) (\bX_2^+\bX) \bX_2^{+} =  (\bX_1^+ \bX)^\top (\bX_2^+\bX)^\top \bX_2^{+}  \qquad &(\text{by $(C4)$})\\
&=(\bX \bX_1^{+} \bX)^\top \bX_2^{+\top} \bX_2^{+} = \bX^\top \bX_2^{+\top} \bX_2^{+}   \qquad &(\text{by $(C1)$})\\
&=  (\bX_2^{+}\bX)^\top \bX_2^{+} =\bX_2^{+}\bX \bX_2^{+} = \bX_2^{+}. \qquad &(\text{by $(C4),(C2)$})\\
\end{aligned}
$$
This shows that any two pseudo-inverses of $\bX$ must be equal. Therefore, the pseudo-inverse is unique.
\end{proof}

\begin{lemma}[Property of pseudo-inverse]\label{lemma:xtxxpllus_pseudo}
For any matrix $\bX\in\real^{n\times p}$, it follows that~\footnote{When $\bX\in\sF^{n\times p}$, this becomes $ \bX^* \bX \bX^+ = \bX^* $.} 
$$ \bX^\top \bX \bX^+ = \bX^\top .$$
\end{lemma}
\begin{proof}[of Lemma~\ref{lemma:xtxxpllus_pseudo}]
From the first Penrose condition of the pseudo-inverse (see \eqref{equation:pseudi-four-equations}), we know:
$
\bX \bX^+ \bX = \bX.
$
Take the  transpose of both sides:
$$
(\bX \bX^+ \bX)^\top = \bX^\top
\qquad \implies\qquad 
\bX^\top (\bX \bX^+)^\top = \bX^\top.
$$
From the third Penrose condition of the pseudo-inverse, $ (\bX \bX^+)^\top = \bX \bX^+ $. Substituting this into the equation above, we get
$
\bX^\top \bX \bX^+ = \bX^\top.
$
This completes the proof.
\end{proof}
This equality reflects the fact that $ \bX^+ $ ``projects'' vectors in the row space of $ \bX $ back onto the column space of $ \bX $ (see below for more details). Specifically:
\begin{itemize}
\item $ \bX^+ $ maps any vector $ \by $ to the least squares solution of the system $ \bX\bbeta = \by $ (we will explore this further in the book).
\item When $ \bX^+ $ is applied to $ \bX $, it effectively reconstructs $ \bX $ in a way that respects its column space and null space.
\end{itemize}
The equality $ \bX^\top \bX \bX^+ = \bX^\top $ encodes the idea that applying $ \bX^+ $ to $ \bX $ does not distort the action of $ \bX^\top $ on vectors.

We are now ready to present the four fundamental subspaces associated with the pseudo-inverse.
\begin{theoremHigh}[Four subspaces in pseudo-inverse]\label{theorem:pseudo-four-basis-space} 
Given the pseudo-inverse $\bX^+$ of $\bX$, the following properties hold:
\begin{itemize}
\item The column space of $\bX^+$ is the same as the row space of $\bX$;

\item The row space of $\bX^+$ is the same as the column space of $\bX$; 

\item The null space of $\bX^+$ is the same as the null space of $\bX^\top$; 

\item The null space of $\bX^{+\top}$ is the same as the null space of $\bX$.
\end{itemize}
The relationships among these four subspaces are illustrated in Figure~\ref{fig:lafundamental5-pseudo}.
\end{theoremHigh}
\begin{proof}[of Theorem~\ref{theorem:pseudo-four-basis-space}]
Since $\bX^+$ is a special rg-inverse, by Lemma~\ref{Lemma:refle-ginverse-properties}, we have 
$$
\begin{aligned}
\cspace(\bX\bX^+) &= \cspace(\bX) \qquad \text{and} \qquad \nspace(\bX^+\bX) = \nspace(\bX)\\
\cspace(\bX^+\bX) &= \cspace(\bX^+)\qquad \text{and} \qquad \nspace(\bX\bX^+) = \nspace(\bX^+).
\end{aligned}
$$
Additionally, from conditions $(C3)$ and $(C4)$ in the definition of pseudo-inverses, we know that:
$$
(\bX^+\bX)^\top = \bX^+\bX \qquad \text{and} \qquad (\bX\bX^+)^\top =\bX\bX^+.
$$
Using  the fundamental theorem of linear algebra (Theorem~\ref{theorem:fundamental-linear-algebra}), we realize that $\cspace(\bX\bX^+)$ is the orthogonal complement to $\nspace((\bX\bX^+)^\top)$, and $\cspace(\bX^+\bX)$ is the orthogonal complement to $\nspace((\bX^+\bX)^\top)$: 
$$
\begin{aligned}
\cspace(\bX\bX^+) \perp \nspace((\bX\bX^+)^\top) &\leadto \cspace(\bX\bX^+) \perp \nspace(\bX\bX^+) \\
\cspace(\bX^+\bX) \perp \nspace((\bX^+\bX)^\top) &\leadto \cspace(\bX^+\bX) \perp \nspace(\bX^+\bX).
\end{aligned}
$$
This implies 
$$
\begin{aligned}
\cspace(\bX) \perp \nspace(\bX^+) \qquad \text{and} \qquad \cspace(\bX^+)\perp \nspace(\bX).
\end{aligned}
$$
That is, $\nspace(\bX^+) = \nspace(\bX^\top)$ and $\cspace(\bX^+) = \cspace(\bX^\top)$. 
By the fundamental theorem of linear algebra, this also implies: $\cspace(\bX^{+\top})=\cspace(\bX)$ and $\nspace(\bX^{+\top})=\nspace(\bX)$.
\end{proof}

\begin{figure}[h!]
\centering
\includegraphics[width=0.98\textwidth]{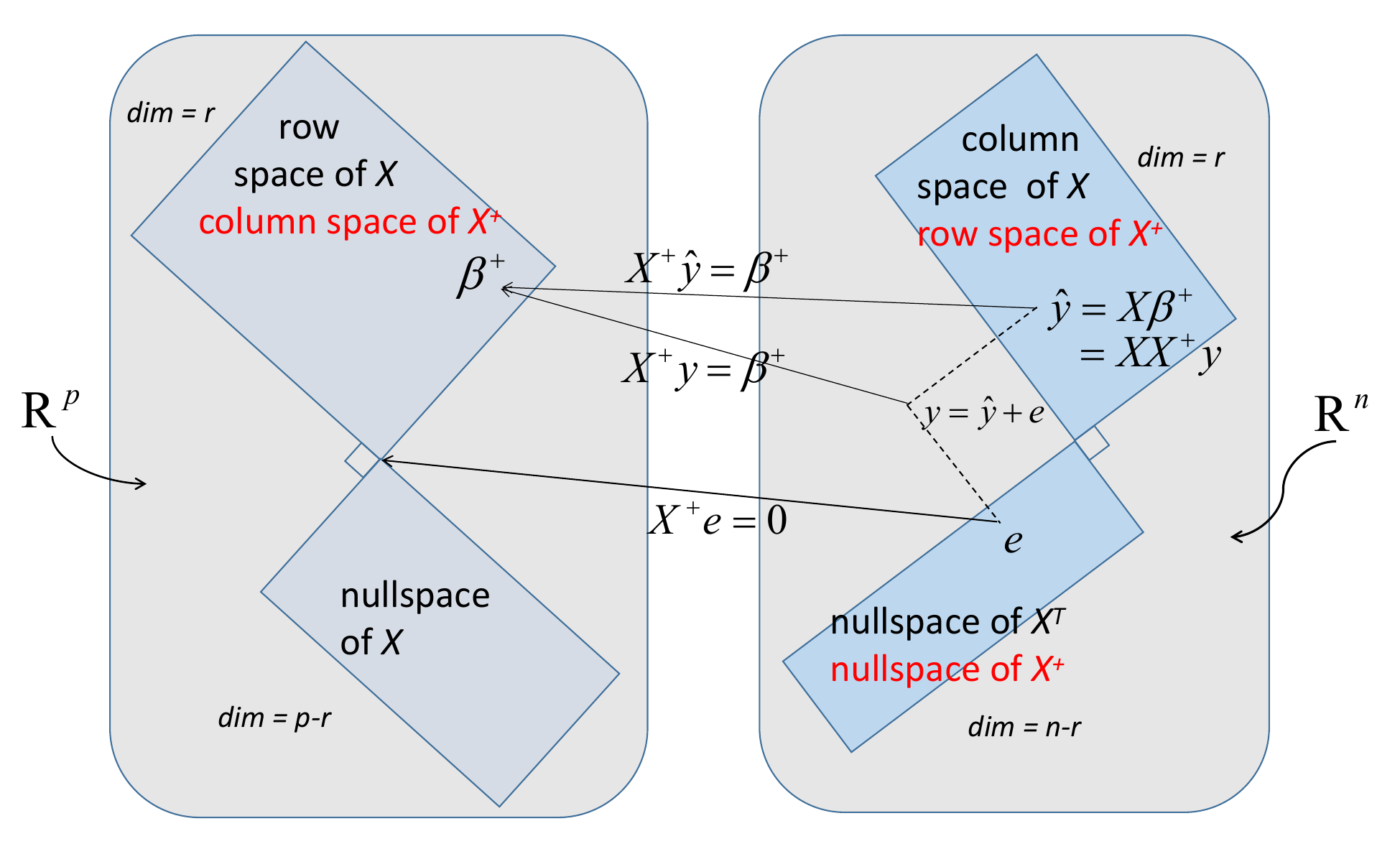}
\caption{Column space and row space of the pseudo-inverse $\bX^+$. $\bX$ transfers from the row space to the  column space. $\bX^+$ maps from the column space to the row space. The decomposition of $\by$ into $\widehatby+\be$ and the transformation to $\bbeta^+$ are discussed in Section~\ref{section:pseudo-in-svd}. \textit{This is a more detailed picture of the pseudo-inverse compared to Figure~\ref{fig:lafundamental4-LS-SVD}}.}
\label{fig:lafundamental5-pseudo}
\end{figure}

To conclude, we compare the properties for different inverses of $\bX$ in Table~\ref{table:different-inverses}.
\begin{table}[h!]
\centering
\resizebox{1.\textwidth}{!}{%}
\begin{tabular}{c|l|l|l}
\hline
& \multicolumn{1}{c|}{g-inverse}                                                                                   & \multicolumn{1}{c|}{rg-inverse}                                                                                                                                                                             & pseudo-inverse                                                                                                                                                                                                                                                                \\ \hline\hline
subspaces & \begin{tabular}[c]{@{}l@{}}$\cspace(\bX\bX^-) = \cspace(\bX)$\\ $\nspace(\bX^-\bX) = \nspace(\bX)$\end{tabular}  & \begin{tabular}[c]{@{}l@{}}$\cspace(\bX\bX_r^-) = \cspace(\bX)$ \\ $\nspace(\bX_r^-\bX) = \nspace(\bX)$\\ $\cspace(\bX_r^-\bX) = \cspace(\bX_r^-)$ \\ $\nspace(\bX\bX_r^-) = \nspace(\bX_r^-)$\end{tabular} & \begin{tabular}[c]{@{}l@{}}$\cspace(\bX\bX^+) = \cspace(\bX)=\cspace(\bX^{+\top})$ \\ $\nspace(\bX^+\bX) = \nspace(\bX)=\nspace(\bX^{+\top})$\\ $\cspace(\bX^+\bX) = \cspace(\bX^+)=\cspace(\bX^\top)$ \\ $\nspace(\bX\bX^+) = \nspace(\bX^+)=\nspace(\bX^\top)$\end{tabular} \\ \hline\hline
rank      & \begin{tabular}[c]{@{}l@{}}$\rank(\bX\bX^-)$\\ $=\rank(\bX^-\bX)$\\ $=\rank(\bX)$\\ $\leq \rank(\bX^-)$\end{tabular} & \begin{tabular}[c]{@{}l@{}}$\rank(\bX_r^-)$\\ $=\rank(\bX)$\\ $=\rank(\bX\bX_r^-)$\\ $=\rank(\bX_r^-\bX)$\end{tabular}                                                                                          & \begin{tabular}[c]{@{}l@{}}$\rank(\bX^+)$\\ $=\rank(\bX)$\\ $=\rank(\bX\bX^+)$\\ $=\rank(\bX^+\bX)$\end{tabular}                                                                                                                                                                  \\ \hline
\end{tabular}
}
\captionof{table}{Comparison of different inverses, presenting the subspaces and ranks of different inverses.}
\label{table:different-inverses}
\end{table}

\paragrapharrow{Pseudo-inverse in different cases.}
We conclude this section by presenting the pseudo-inverse for various types of matrices. Specifically, we define it in the following cases:
\paragraph{Case $n>p=r$.} That is, the matrix $\bX\in\real^{n\times p}$ has full column rank. In this case, 
$\bX^\top\bX$ is a $p\times p$ invertible matrix. And we define the left-pseudo-inverse:
$$
{\text{left-pseudo-inverse} = \bX^+ = (\bX^\top\bX)^{-1}\bX^\top},
$$
which satisfies
$
\bX^+ \bX =  (\bX^\top\bX)^{-1}\bX^\top\bX = \bI_p.
$
But 
$
\bX\bX^+ = \bX(\bX^\top\bX)^{-1}\bX^\top \neq\bI.
$
We can also show that $(\bX^+)^+ = \bX$. If $n>p=r$, we have 
$$
\begin{aligned}
(\bX^+)^+ 
&= [(\bX^\top\bX)^{-1}\bX^\top]^+ 
= \bX^{+\top}(\bX^+\bX^{+\top})^{-1}  \\
&= \left[(\bX^\top\bX)^{-1}\bX^\top\right]^{\top}\left\{\left[(\bX^\top\bX)^{-1}\bX^\top\right]\left[(\bX^\top\bX)^{-1}\bX^\top\right]^{\top}\right\}^{-1}\\
&=\bX (\bX^\top\bX)^{-1}\left\{(\bX^\top\bX)^{-1}\bX^\top\bX(\bX^\top\bX)^{-1}   \right\}^{-1}
=\bX.
\end{aligned}
$$

\paragraph{Case $p>n=r$.} That is, the matrix $\bX$ has full row rank.
In this case, 
$\bX\bX^\top$ is an $n\times n$ invertible matrix. And we define the right-pseudo-inverse:
$$
{\text{right-pseudo-inverse} = \bX^+ = \bX^\top(\bX\bX^\top)^{-1}},
$$
which satisfies 
$$
\bX\bX^+ = \bX\bX^\top(\bX\bX^\top)^{-1} = \bI_n.
$$
But 
\begin{equation}\label{equation:right-inverse-ba-appendix}
\bX^+ \bX =  \bX^\top(\bX\bX^\top)^{-1}\bX \neq \bI.
\end{equation}
Similarly, we can show $(\bX^+)^+ = \bX$ if $p>n=r$.

\paragraph{Case rank-deficient.} We delay the pseudo-inverse for rank-deficient matrices in the next section via the SVD.

\index{Rank-deficient}
\index{Rank-deficiency}

\paragraph{Case $n=p$ with full rank.} $\bX$ is a square invertible matrix,  then both the left and right pseudo-inverses reduce to the standard inverse of $\bX$:
$$
\begin{aligned}
\text{left-pseudo-inverse} &= \bX^+ = (\bX^\top\bX)^{-1}\bX^\top=\bX^{-1} \bX^{-\top} \bX^\top = \bX^{-1}; \\
\text{right-pseudo-inverse} &= \bX^+ = \bX^\top(\bX\bX^\top)^{-1} = \bX^\top\bX^{-\top}\bX^{-1} = \bX^{-1}.
\end{aligned}
$$
%\end{mdframed}

\index{Pseudo-inverse in SVD}
\subsection{Pseudo-Inverse in SVD}\label{section:pseudo-in-svd}
Given the SVD of a matrix $\bX$, we provide further discussion on the pseudo-inverse in different cases.
For the full SVD of matrix $\bX = \bU\bSigma\bV^\top$,  we consider the following cases:
\paragraph{Case $n>p=r$.} Since the matrix $\bX$ has independent columns, the left-pseudo-inverse can be obtained by 
$$
\begin{aligned}
\text{left-pseudo-inverse} 
&= \bX^+ = (\bX^\top\bX)^{-1}\bX^\top 
= (\bV \bSigma^\top \bU^\top\bU\bSigma\bV^\top)^{-1}\bV \bSigma^\top \bU^\top\\
&= \bV (\bSigma^\top\bSigma)^{-1}\bV^\top \bV \bSigma^\top \bU^\top 
= \bV[(\bSigma^\top\bSigma)^{-1}\bSigma^\top] \bU^\top 
= \bV \bSigma^+ \bU^\top,
\end{aligned}
$$
where the last equality follows because $\bSigma^+=(\bSigma^\top\bSigma)^{-1}\bSigma^\top$.
\paragraph{Case $p>n=r$.} Since the matrix $\bX$ has independent rows, the right-pseudo-inverse can be obtained by 
$$
\begin{aligned}
\text{right-pseudo-inverse} 
&= \bX^+ = \bX^\top(\bX\bX^\top)^{-1}
= (\bU\bSigma\bV^\top)^\top[(\bU\bSigma\bV^\top)(\bU\bSigma\bV^\top)^\top]^{-1}\\
&= \bV\bSigma^\top\bU^\top(\bU\bSigma\bV^\top\bV\bSigma^\top\bU^\top)^{-1} 
= \bV\bSigma^\top\bU^\top\bU^{-\top}(\bSigma\bSigma^\top)^{-1}\bU^{-1} \\
&= \bV\bSigma^\top(\bSigma\bSigma^\top)^{-1}\bU^{-1} 
=\bV \bSigma^+ \bU^\top,
\end{aligned}
$$
where the last equality follows because $\bSigma^+=\bSigma^\top(\bSigma\bSigma^\top)^{-1}$.

\paragraph{Case rank-deficient.} $\bX^+=\bV \bSigma^+ \bU^\top$,
where the upper-left side of $\bSigma^+ \in \real^{p\times n}$ is a diagonal matrix $\diag(\frac{1}{\sigma_1}, \frac{1}{\sigma_2}, \ldots, \frac{1}{\sigma_r})$. It can be easily verified that this definition of the pseudo-inverse satisfies the four conditions in Equation~\eqref{equation:pseudi-four-equations}.

In all cases, we have $\bSigma^+$ as the pseudo-inverse of $\bSigma$ with $1/\sigma_1, 1/\sigma_2, \ldots, 1/\sigma_r$ on its diagonal. We thus conclude the pseudo-inverse from SVD:
\begin{equation}\label{equation:pseudo_use_svd}
\bX^+ = \bV 
\begin{bmatrix}
	\bSigma^{-1}_1 & \bzero \\
	\bzero & \bzero
\end{bmatrix} 
\bU^\top,
\quad 
\text{with }
\bSigma = 
\begin{bmatrix}
	\bSigma_1 & \bzero \\
	\bzero & \bzero
\end{bmatrix} 
\text{ and }
\bSigma_1\in\real^{r\times r}.
\end{equation}
See also Table~\ref{table:pseudo-inverse-svd}.
\begin{table}[h!]\centering
\begin{tabular}{c|c|c|c|c}
\hline
& $\bX$                & $\bX^\top$           & $\bX^+$                & $\bX^{+\top}$          \\ \hline\hline
SVD & $\bU\bSigma\bV^\top$ & $\bV\bSigma\bU^\top$ & $\bV\bSigma^+\bU^\top$ & $\bU\bSigma^+\bV^\top$ \\ \hline
\end{tabular}\caption{Pseudo-inverse in SVD.}\label{table:pseudo-inverse-svd}
\end{table}
If $\bX$ is nonsingular, then $\bX^+ = \bX^{-1}$, so \eqref{equation:pseudo_use_svd} is a generalization of the usual inverse.
The pseudo-inverse of a scalar is
$$
\sigma^+ = 
\begin{cases}
1/\sigma, & \text{if } \sigma \neq 0; \\
0, & \text{if } \sigma = 0.
\end{cases}
$$
This shows the important fact that the pseudo-inverse $\bX^+$ is not a continuous function of $\bX$, unless we allow only perturbations that do not change the rank of $\bX$. The pseudo-inverse has the property
$$
\bX^+ = \lim_{\delta \to 0} (\bX^\top \bX + \delta \bI)^{-1} \bX^\top.
$$

Using the SVD representation of pseudo-inverses in \eqref{equation:pseudo_use_svd}, we can derive additional properties of the pseudo-inverse of a matrix $\bX$.
\begin{lemma}[Properties of pseudo-inverse using SVD]\label{lemma:prop_pseu_use_svd}
Let $\bX\in\real^{n\times p}$. The following properties of the pseudo-inverse follow from \eqref{equation:pseudo_use_svd}.
\begin{enumerate}[(i)]
\item $\bX^+ = (\bX^\top \bX)^+ \bX^\top = \bX^\top (\bX\bX^\top)^+$.
\item \label{equation:pseu_full_col_trans} When $\rank(\bX) = p$, this becomes $\bX^+ = (\bX^\top \bX)^{-1} \bX^\top, \ (\bX^\top)^+ = \bX(\bX^\top \bX)^{-1}$.
\item $(\bX^+)^+ = \bX$.
\item $(\alpha \bX)^+ = \alpha^+ \bX^+$.
\item $(\bX^+)^\top = (\bX^\top)^+$.
\item $(\bX^\top \bX)^+ = \bX^+ (\bX^\top)^+$.
\item $\bX$, $\bX^\top$, $\bX^+$, and $\bX^+ \bX$ all have rank equal to $\trace(\bX^+ \bX)$.
\item If $\bU$ and $\bV$ are orthogonal, $(\bU \bX \bV^\top)^+ = \bV \bX^+ \bU^\top$.
\item If $\bX = \sum_i \bX_i$, where $\bX_i \bX_j^\top = \bzero$, $\bX_i^\top \bX_j = \bzero$, $i \neq j$, then $\bX^+ = \sum_i \bX_i^+$.
\item If $\bX$ is normal ($\bX\bX^\top = \bX^\top \bX$), then $\bX^+ \bX = \bX\bX^+$ and $(\bX^n)^+ = (\bX^+)^n$.
\end{enumerate}
\end{lemma}

For the pseudo-inverse, the relations $\bX\bX^+ = \bX^+ \bX$ and $(\bX \bY)^+ = \bY^+ \bX^+$ are not in general true. For example, let $\bX = \footnotesize [1 , 0] $ and $\bY = [1 , 1]^\top$. Then $\bX \bY = 1$, but
$ \bY^+ \bX^+ = \frac{1}{2} [1 , 1] \begin{bmatrixfoot} 1 \\ 0 \end{bmatrixfoot} =  \frac{1}{2}. $
Necessary and sufficient conditions for the identity $(\bX \bY)^+ = \bY^+ \bX^+$ to hold were established by \citet{greville1966note}. The following theorem provides an important sufficient condition under which this equality holds.

\begin{theoremHigh}[Sufficient Condition for $(\bX \bY)^+ = \bY^+ \bX^+$]\label{theorem:suff_pro_pseudo}
If $\bX \in \real^{n \times p}$, $\bY \in \real^{p \times m}$, and $\rank(\bX) = \rank(\bY) = p$, 
then
$ (\bX \bY)^+ = \bY^+ \bX^+ = \bY^\top (\bY \bY^\top)^{-1} (\bX^\top \bX)^{-1} \bX^\top. 
$
\end{theoremHigh}
\begin{proof}[of Theorem~\ref{theorem:suff_pro_pseudo}]
The last equality follows from Lemma~\ref{lemma:prop_pseu_use_svd}.\ref{equation:pseu_full_col_trans}. The first equality follows from the proof of Lemma~\ref{lemma:existence-of-pseudo-inverse} and  is verified by showing that the four Penrose conditions are satisfied.
\end{proof}

%\paragrapharrow{Another way to see the orthogonal projection in pseudo-inverse via SVD.}
%From the pseudo-inverse via SVD, we can provide an alternative way to see the orthogonal projection in  pseudo-inverses.
%We have shown previously that $\bH=\bX\bX^+$ is an orthogonal projection, so we only need to show that it projects onto $\cspace(\bX)$. For any vector $\by\in \real^n$, we have 
%$$
%\bH\by= \bX\bX^+\by = \bX\bbeta^+,
%$$
%which is a linear combination of columns of $\bX$. Thus, $\cspace(\bH)\subseteq \cspace(\bX)$.
%Moreover, since $\bH$ is an orthogonal projection, by Lemma~\ref{lemma:rank-of-symmetric-idempotent2}, we have $\rank(\bH)=\trace(\bH)=\trace(\bX\bX^+) = \trace(\bU\bSigma\bV^\top \bV\bSigma^+\bU^\top)=\trace(\bU\bSigma\bSigma^+\bU^\top)=r$, where $\bU\bSigma\bV^\top$ is the SVD of $\bX$. Then, $\cspace(\bH)= \cspace(\bX)$.
%
%Similarly, we can prove $\bP= \bX^+\bX$ is the orthogonal projection onto the row space of $\bX$.

\paragrapharrow{Another way to see the subspaces in pseudo-inverse via SVD.}
Using SVD, we provide an alternative perspective on demonstrating the four fundamental subspaces associated with the pseudo-inverse introduced in Theorem~\ref{theorem:pseudo-four-basis-space} and Figure~\ref{fig:lafundamental5-pseudo}.
From Lemma~\ref{lemma:rank-of-symmetric}, consider the symmetric matrix $\bX^+ \bX^{+\top}$ and its spectral decomposition: $\bX^+ \bX^{+\top} = \bV(\bSigma^\top\bSigma)^{-1}\bV^\top$.
The column space $\cspace(\bX^+ \bX^{+\top})$ is spanned by the eigenvectors of this decomposition. 
Therefore, the set $\{\bv_1,\bv_2 \ldots, \bv_r\}$ forms an orthonormal basis for $\cspace(\bX^+ \bX^{+\top})$.

We now establish the following results:
\begin{enumerate}[(i)]
\item Since $\bX^+ \bX^{+\top}$ is symmetric, its row space coincides with its column space.

\item Every column of $\bX^+ \bX^{+\top}$ is a linear combination of the columns of $\bX^+$. Therefore, the column space of $\bX^+ \bX^{+\top}$ $\subseteq$ the column space of $\bX^+$, i.e., $\cspace(\bX^+ \bX^{+\top}) \subseteq \cspace(\bX^+)$. 

\item  $\rank(\bX^+ \bX^{+\top}) = \rank(\bX^+)$ by Lemma~\ref{lemma:rank-of-ata-x}. 

Consequently, the row space of $\bX^+ \bX^{+\top}$ = the column space of $\bX^+ \bX^{+\top}$ =  the column space of $\bX^+$, i.e., $\cspace(\bX^+ \bX^{+\top}) = \cspace(\bX^+ )$. \textcolor{black}{Consequently, $\{\bv_1, \bv_2,\ldots, \bv_r\}$ constitutes an orthonormal basis of $\cspace(\bX^+)$}. We also proved in Theorem~\ref{theorem:svd-four-orthonormal-Basis} that $\{\bv_1, \bv_2,\ldots, \bv_r\}$ is an orthonormal basis of the row space of $\bX$ (i.e., basis of $\cspace(\bX^\top)$). Thus, $\cspace(\bX^+)=\cspace(\bX^\top)$, as shown in Figure~\ref{fig:lafundamental5-pseudo}.
\end{enumerate}

Similarly, if we apply this process to $ \bX^{+\top}\bX^+$, we can show that the row space of $\bX^+$ is equal to the column space of $\bX$, and the null space of $\bX^+$ is equal to the null space of $\bX^\top$.

Now, consider a vector $\bbeta^+$ in the row space of $\bX$. 
Since  $\{\bv_1, \bv_2, \ldots, \bv_r\}$ forms an orthonormal basis for the row space of $\bX$, we can express $\bbeta^+$ as $\bbeta^+ = \sum_{i=1}^{r} x_i \bv_i$   (see Theorem~\ref{theorem:svd-four-orthonormal-Basis}). For a vector $\widehatby = \bX\bbeta^+$ in the column space of $\bX$, we have $\widehatby = \bU\bSigma\bV^\top \bbeta^+$ and 
$$
\bX^+\widehatby = \bV\bSigma^+\bU^\top \bU\bSigma\bV^\top \bbeta^+ = \bV\bSigma^+\bSigma\bV^\top \bbeta^+ =\left(\sum_{i=1}^{r} \bv_i \bv_i^\top\right) \left(\sum_{i=1}^{r} x_i \bv_i\right) = \sum_{i=1}^{r} x_i \bv_i = \bbeta^+.
$$

For any vector $\be$ in the null space of $\bX^\top$, we have $\bX^+\be=\bzero$, since $\nspace(\bX^+)=\nspace(\bX^\top)$. Any vector $\by \in \real^n$ can be decomposed into $\by = \widehatby+\be$, where $\widehatby$ is a vector in the column space of $\bX$, and $\be$ is a vector in the null space of $\bX^\top$. That is,
$$
\bX^+ \widehatby = \bX^+\by = \bbeta^+,
$$
where $\bbeta^+$ lies in the row space of $\bX$.

In conclusion, for any vector $\bbeta^+$ in row space of $\bX$, we have 
$$
\widehatby = \bX\bbeta^+ 
\quad \implies\quad 
\bX^+\widehatby=\bbeta^+,
$$
and the relationship is illustrated  in Figure~\ref{fig:lafundamental5-pseudo}.

\begin{problemset}
%\item \label{problem:psd_eigen} \textbf{Eigenvalue characterization theorem.} Prove that a matrix $\bX$ is positive definite if and only if it possesses exclusively \textit{positive eigenvalues}. Similarly, a matrix $\bX$ is positive semidefinite if and only if it exhibits solely \textit{nonnegative eigenvalues}. 

\item \label{prob:tr_de_pd} \textbf{Trace, det of PD/PSD/ND matrices.} Let $\bA$ be positive definite (resp. positive semidefinite). Show that $\trace(\bA), \det(\bA)$, and the principal minors of $\bA$ are all positive (resp. nonnegative). Moreover, $\trace(\bA)=0$ if and only if $\bA=\bzero$. 
Let $\bB\in\real^{n\times n}$ be negative definite. Show that $\trace(\bB)$ is negative; $\det(\bB)$ is negative for odd $n$ and positive for even $n$.
\textit{Hint: Use Theorem~\ref{theorem:eigen_charac}.}

\item Prove Remark~\ref{remark:equiva_nonsingular} and Remark~\ref{remark:equiva_singular}.

\item Demonstrate that the vector $\ell_2$ norm, the matrix Frobenius norm, and the matrix spectral norm satisfy the three criteria outlined in Definition~\ref{definition:matrix-norm}.

\item \label{problem:multiGauss} Suppose we can generate the univariate Gaussian variable $\normal(0, 1)$. Provide a way to generate the multivariate Gaussian variable $\normal(\bmu, \bSigma)$, where $\bSigma=\bC\bC^\top$, $\bmu\in\real^n$, and $\bSigma\in\real^{n\times n}$. \textit{Hint: if $\rx_1, \rx_2, \ldots, \rx_n$ are i.i.d. from $\normal(0, 1)$, and let $\rvx=[\rx_1, \rx_2, \ldots, \rx_n]^\top$, then it follows that $\bC\rvx +\bmu\sim \normal(\bmu, \bSigma)$}.

\item Prove Lemma~\ref{lemma:affine_mult_gauss} and Lemma~\ref{lemma:rotat_multi_gauss}.

\item Following the Jacobian in the change-or variables formula for the Gamma distribution and inverse-Gamma distribution, and the definition of the Chi-squared distribution provided in Definition~\ref{definition:chisquare_dist}, derive the Jacobian in the change-of-variables formula for the Chi-squared distribution and the inverse-Chi-squared distribution.

\item \label{problem:kl_vae}
\textbf{KL of Gaussians.} Given two probability distribution function $P(\bx)$ and $Q(\bx)$, we denote  the \textit{Kullback-Leibler (KL) divergence} between $P$ and $Q$ by $\KL[P \parallel Q] = \int P(\bx) \ln \left( \frac{P(\bx)}{Q(\bx)} \right) d\bx \geq 0$, where the equality is obtained only when $P=Q$.
Given $q(\bx)=\normal( \bmu, \diag(\bsigma^2))$ and $p(\bx)=\normal( \bzero_D, \bI_D)$ where $\bmu, \bsigma^2\in\real^D$ (Definition~\ref{sec:multi_gaussian_conjugate_prior}). Show that $\KL[q\parallel p]=\frac{1}{2}\sum_{i=1}^{D}(\mu_i^2+\sigma^2_i-\ln \sigma^2_i-1)$. This expression is commonly used as the KL loss  in fitting variational autoencoders \citep{lu2021bayes}.

\item \textbf{KL of Gaussians.} Suppose $p(x)=\normal(\mu_1, \sigma_1^2)$ and $q(x)=\normal(\mu_2,\sigma_2^2)$. Show that 
$$
\KL[p \parallel q ] = \ln\frac{\sigma_2}{\sigma_1} +\frac{\sigma_1^2+(\mu_1-\mu_2)^2}{2\sigma_2^2}-\frac{1}{2}.
$$ 
Consider the multivariate case, suppose $\normal_1(\bx)=\normal(\bmu_1, \bSigma_1)$ and $\normal_2(\bx)=\normal(\bmu_2,\bSigma_2)$ (Definition~\ref{sec:multi_gaussian_conjugate_prior}). Show that 
$$
\KL[\normal_1\parallel \normal_2] = 
\frac{1}{2}\ln\abs{\bSigma_2\bSigma_1^{-1}}+\frac{1}{2}\trace\bSigma_2^{-1}
\big( (\bmu_1-\bmu_2)(\bmu_1-\bmu_2)^\top +\bSigma_1-\bSigma_2 \big).
$$
More generally, consider a general distribution $p(\bx)$ and a multivariate Gaussian $\normal(\bx)=\normal(\bmu, \bSigma)$ with $\bx\in\real^D$. Show that 
$$
\KL[p\parallel \normal ]=
\int \frac{1}{2}(\bx-\bmu)^\top\bSigma^{-1}(\bx-\bmu) d\bx 
+
\frac{1}{2}\ln \abs{\bSigma}+\frac{D}{2}\ln2\pi + \int p(\bx)\ln p(\bx) d\bx.
$$

\item Given two Bernoulli distributions $p(x)=\bernoulli(x\mid p)$ and $q(x)=\bernoulli(x\mid q)$ (Equation~\eqref{equation:bbern_dist}), show that 
$$
\KL[p \parallel q] = p\ln\frac{p}{q} + (1-p) \ln \frac{1-p}{1-q}.
$$

\item \label{problem:entropy_mgau} \textbf{Entropy of Gaussians.}  A close quantity related to KL divergence is the \textit{entropy}. The entropy $\entropy[p(\bx)]$ of a distribution $p(\bx)$ is defined as 
$$
\entropy[p(\bx)] \triangleq -\int p(\bx)\ln p(\bx) d\bx.
$$
Given a multivariate Gaussian distribution $\rvx\sim \normal(\bmu,\bSigma)$ (Definition~\ref{sec:multi_gaussian_conjugate_prior}) where $\rvx\in\real^D$, show that 
$$
\entropy[\normal(\bmu,\bSigma)] = \frac{1}{2}\ln \abs{\bSigma}+ \frac{D}{2} \ln(2\pi e).
$$

\item \textbf{Properties of expectation.}
Let $\rx$ and $\ry$ be two random variables, and let $a, b$ be scalars. Show that $\Exp[a\rx+b\ry] = a\cdot\Exp[\rx]+b\cdot \Exp[\ry]$.
Given further a  function $h$, show that 
$$
\Exp[h(\rx)] = \sum_{x} h(x) \Pr(x)
\qquad \text{and}\qquad 
\Exp[h(\rx)] = \int_{-\infty}^{\infty} h(x) d F(x),
$$
in the discrete and continuous cases, respectively.

\item \label{prob:prop_expcond}\textbf{Properties of expectation of conditionals.}
Let $\rx$ and $\ry$ be two random variables, and let $h$ be a  function.
Show that
\begin{enumerate}
\item Note that $\Exp[\rx\mid\ry]$ is a function of $\ry$. However, if $\rx$ is independent of $\ry$, then $\Exp[\rx\mid \ry] = \rx$.
\item $\Exp[c\mid \rx] =\rx$, where $c$ is a constant.
\item Linearity: $\Exp[a\rx_1 + b\rx_2\mid \ry] = a\cdot\Exp[\rx_1\mid \ry] + b\cdot\Exp[\rx_2\mid \ry]$.
\item Conditional constant: $\Exp[h(\ry)\rx\mid \ry] = h(\ry)\Exp[\rx\mid \ry]$, where $h(\ry)$ is called a \textit{conditional constant} w.r.t. $\ry$.
\item Monotonicity: if $\rx_1 \leq \rx_2$, then $\Exp[\rx_1\mid \ry] \leq \Exp[\rx_2\mid \ry]$.
\item Tower property: $\Exp\left[\Exp[\rx\mid \ry]\mid h(\ry)\right] = \Exp[\rx\mid h(\ry)]$; that is, $h(\ry)$ conveys information at most as $\ry$.
\item Unbiasedness: $\Exp\big\{\Exp[h(\rx,\ry)\mid \ry]\big\} = \Exp[h(\rx,\ry)]$; specially, $\Exp\big[\Exp[\rx\mid \ry]\big] = \Exp[\rx]$.
\item Least squares: $\Exp\big[\big(\ry - \Exp[\ry\mid \rx]\big)^2\big] \leq \Exp\big[\big(\ry - h(\rx)\big)^2\big]$ for any function $h$. This also means  $g(\rx)\triangleq  \Exp[\ry\mid \rx]$ is the best estimate in the least squares sense.
\end{enumerate}

\item \textbf{Sum of random variables by convolution.}
Let $ \rx $ and $ \ry $ be continuous random variables with probability density functions $ f_{\rx} $ and $ f_{\ry} $. Show that the density function of $ \rx + \ry $ is the convolution of $ f_{\rx} $ with $ f_{\ry} $:
$$
f_{\rx+\ry}(u) = \int_{-\infty}^{+\infty} f_{\rx}(u - v) f_\ry(v) \, dv.
$$

\item \textbf{Properties of variance and correlation.}
Let $\rx, \rx_1, \rx_2, \ry, \rvx$ be random variables or vectors, and  let $a,b$ be constants. 
Show that 
\begin{itemize}
\item Let $\bOmega$ be a real symmetric matrix. Then $\bOmega$ is positive semidefinite (Definition~\ref{definition:psd-pd-defini}) if and only if $\bOmega$ is the covariance matrix of some random vector $\rvx$.
\item $\Var[\rx] = \Exp[\rx^2] - (\Exp[\rx])^2 = \Cov[\rx,\rx]$.
\item $\Var[a\rx + b] = a^2 \Var[\rx]$.
\item $\Var\big[\sum_i \rx_i\big] = \sum_i \Var[\rx_i] + \sum_{i \neq j} \Cov[\rx_i, \rx_j]$.
\item $\Cov[\rx_1, \rx_2] = \Exp[\rx_1 \rx_2] - \Exp[\rx_1]\Exp[\rx_2]$.
\item $\Cov[a\rx_1 + b\rx_2, \ry] = a\cdot\Cov[\rx_1, \ry] + b\cdot \Cov[\rx_2, \ry]$; that is, covariance is linear in one variable.
\item If $\Exp[\rx_1^2] + \Exp[\rx_2^2] < \infty$, then the following are equivalent:
\begin{enumerate}
\item[(i)] $\Exp[\rx_1 \rx_2] = \Exp[\rx_1]\Exp[\rx_2]$;
\item[(ii)] $\Cov[\rx_1, \rx_2] = 0$;
\item[(iii)] $\Var[\rx_1 \pm \rx_2] = \Var[\rx_1] + \Var[\rx_2]$.
\end{enumerate}
Note that independence will imply these three last properties, but none of these properties imply independence.
\item Let $h$ be a nondecreasing function such that $\Exp[\rx^2]<\infty$ and $\Exp[h(\rx)^2]<\infty$. Then $\Cov[\rx, h(\rx)]>0$.
\end{itemize}

\item \textbf{Conditional variance and law of total variance.}
Let $\rx$ and $\ry$ be two random variables. The \textit{condition variance} of $\rx$ given $\ry$ is defined as 
\begin{equation}
\Var[\rx \mid \ry] \triangleq \Exp\big[ \big(\rx - \Exp[\rx \mid \ry]\big)^2 \mid \ry \big] =\Exp[\rx^2\mid \ry ] - (\Exp[\rx \mid \ry ])^2.
\end{equation}
The conditional variance tells us how much variance is left if we use $\Exp[\rx\mid \ry]$ to ``predict" $\rx$.
Prove the \textit{law of total variance}:
\begin{equation}
\Var[\rx] = \Exp\big[\Var[\rx\mid \ry]\big] + \Var\big[\Exp[\rx\mid \ry]\big].
\end{equation}
\textit{Hint: $\Var[\rx] = \Exp[\rx^2] -(\Exp[\rx])^2 = \Exp[\Exp[\rx^2\mid \ry ]] -(\Exp[\Exp[\rx\mid \ry]])^2$ by the unbiasedness property in Problem~\ref{prob:prop_expcond}.}
\end{problemset}

\newpage
\chapter{Least Squares Approximations}\label{chapter:ls_approx}
\begingroup
\hypersetup{
linkcolor=structurecolor,
linktoc=page,  % page: only the page will be colored; section, all, none etc
}
\minitoc \newpage
\endgroup
\section{Least Squares Approximations}\label{section:ls_approx1}
\lettrine{\color{caligraphcolor}L}
Linear models is a fundamental technique in solving regression problems, and its core method is the least squares approximation, which aims to minimize the sum of squared errors between predicted and observed values. 
This approach is particularly appropriate when the goal is to estimate the regression function that minimizes the expected squared error loss.
linear models have found wide application across various domains. Examples include decision making \citep{dawes1974linear}, time series analysis \citep{christensen1991linear, lu2017machine, lu2022reducing}, and fields such as production science, social science, and soil science \citep{fox1997applied, lane2002generalized, schaeffer2004application, mrode2014linear}.

We consider the  system 
\begin{equation}\label{equation:ls_nobias}
\by = \bX\bbeta,
\end{equation}
where $\bX\in \real^{n\times p}$ represents the input data matrix (or \textit{predictors}, \textit{predictor variables}, \textit{covariates}, \textit{features}).
These predictors may be numerical by nature, or derived from categorical variables through encoding.
In general, each column will be a function of the actual covariates $\bz_i$; for example $\bx_i = \boldsymbol{\varphi}(\bz_i), i\in\{1,2,\ldots,p\}$.~\footnote{For convenience, we slightly abuse the notation by letting $\bx_i$ denote the $i$-th row of the matrix $\bX$. 
In general, however, we use $\bx^{(i)}$ to represent rows of $\bX$ throughout this book.}
And $\by\in \real^n$ is the observation vector (or \textit{target, response, output variables, outcomes}), $n$ represents the number of observations (sample size), and $p$ denotes the number of features (dimension value). 
The vector $\bbeta$ constitutes a vector of weights of the linear model, which is called the \textit{coefficient vector} or \textit{weight vector}. 
In practice, a \textit{bias term} (also called the \textit{intercept})  is often included by adding a column of ones as the first column of $\bX$. This allows the least squares method to solve:
\begin{equation}\label{equation:ls-bias}
\widetildebX \widetilde{\bbeta} 
= 
[\bm{1} ,\bX ] 
\begin{bmatrix}
\beta_0\\
\bbeta
\end{bmatrix}
= \by .
\end{equation}
Equivalently, for each data point $i\in\{1,2,\ldots, n\}$, we have 
$$
y_i = \beta_0 +\beta_1x_{i1} + \beta_2 x_{i2} +\ldots +\beta_{p-1} x_{i,p-1}.
$$
For simplicity, we will denote $\widetildebX$ and $\widetilde{\bbeta}$ simply as $\bX$ and $\bbeta$, respectively, in all subsequent discussions unless otherwise specified.

%It often happens that $\by = \bX\bbeta$ has no solution. The usual reason is: too many equations, i.e., the matrix has more rows than columns ($n\geq p$, such matrix equations are called \textit{overdetermined or overconstrained systems}).
%When the error $\be$ is as small as possible in the sense of the mean squared error (MSE), $\widehatbbeta$ is a \textit{least squares (LS, or ordinary least squares, OLS)} solution, i.e., $\normtwobig{\by-\bX\widehatbbeta}^2$ is minimized:
%\begin{equation}\label{equation:leas_prob_eq}
%\widehatbbeta = \argmin_{\bbeta\in\real^p} \normtwobig{\by-\bX{\bbeta}}^2.
%\end{equation}

\paragrapharrow{General thoughts.}
In many practical applications, we want to find an approximate solution to a problem or set of equations that, for noise reasons or whatever other reasons, does not have a solution, or not unrelatedly does not have a unique solution. A canonical example of this is given by the very \textit{overdetermined} (i.e., overconstrained) \textit{least squares (LS, or ordinary least squares, OLS)} problem or called the \textit{large sample problem} ($n\gg p$ in \eqref{equation:ls_nobias}), and this will be our focus for the next several sections.

The least squares method originated from the need to reduce the impact of measurement errors when fitting a mathematical model to observed data. One way to achieve this is by using more measurements ($n$) than unknown parameters ($p$) in the model.

While some (but not all) of the concepts we discuss can also apply to \textit{underdetermined} (i.e., \textit{underconstrained}) LS problems (see Section~\ref{section:ls-via-svd}), we will primarily focus on the simpler and more commonly used case of overdetermined LS problems.

Let $\bX \in \real^{n \times p}$ and $\by \in \real^n$ be given. If $n \gg p$, in which case there are significantly more rows/constraints than columns/variables, then in general, there does not exist a vector $\bbeta$ such that $\bX\bbeta = \by$. 
Define the column space of $\bX$, denoted $\cspace(\bX)$, as $\{\bX\bgamma\mid  \forall\, \bgamma \in \real^p\}$.
Thus, the meaning of $\by = \bX\bbeta$ has no solution is that $\by$ has a part that sits outside the column space of $\bX$.
That is, $\by \in \real^n$, but $\cspace(\bX)$ is a $p$-dimensional subspace of $\real^n$, and so with even a little noise, numerical instability, etc., there will be a part of $\by$ that is not captured as a linear combination of the columns of $\bX$.
In other words, the error $\be = \by -\bX\bbeta$ cannot be reduced to zero.

In this case, a popular way to find the ``best'' vector $\bbeta$ such that $\bX\bbeta \approx \by$ is to minimize the norm of the residuals, i.e., to solve $\min_{\bbeta \in \real^p} \norm{\by - \bX\bbeta}$, where $\norm{\cdot}$ is some norm. The most popular choice is the Euclidean or $\ell_2$ norm (Definition~\ref{definition:vec_l2_norm}). 
In this case, the LS problem is to minimize the sum of squares of the residual, i.e., to solve
\begin{equation}\label{equation:ls_l2norm}
\min_{\bbeta \in \real^p} \normtwo{\by - \bX\bbeta}.
\end{equation}
If we let $\bX^+$ denote the  pseudo-inverse of $\bX$ (Section~\ref{section:pseudo-inverse}), then the solution to this minimization problem is:
\begin{equation}\label{equation:ls_solu_pseudo}
	\widehatbbeta = \bX^+ \by = \argmin_{\bbeta\in\real^p} \normtwobig{\by-\bX{\bbeta}}.
\end{equation}
The vector $\widehatbbeta$ is called a {least squares} solution, i.e., $\normtwobig{\by-\bX\widehatbbeta}^2$ is minimized.

Actually, we should note that $\bbeta = \bX^+ \by + \bxi$, where $\bxi \perp \cspace(\bX^\top)$, i.e., where $\bxi \in \real^n$ is any vector orthogonal to the row span of $\bX$ ($\bxi\in\nspace(\bX)$), solves the LS problem given in \eqref{equation:ls_solu_pseudo}.
Among all such solutions, the one given by $\bbeta = \bX^+ \by$  is the minimal ($\ell_2$) norm solution to the LS problem. 
We will be interested in working with this shortest or minimum-norm solution (Theorems~\ref{theorem:qr-for-ls-urv} and \ref{theorem:svd-deficient-rank}).

In most practical settings, the matrix $\bX$ will have full column rank, especially when working with real-world data, which often naturally avoids collinearity, or after preprocessing steps that ensure linear independence among the columns. In this case, the least squares solution simplifies to:
\begin{equation}
\widehatbbeta =\bX^+\by \equiv (\bX^\top\bX)^{-1}\bX^\top\by.
\end{equation}

\index{Loss function}
\index{Objective function}
\paragrapharrow{Numerical methods and loss functions.}
In \eqref{equation:ls_solu_pseudo}, we show that $\widehatbbeta$ minimizes the residual norm: $\min_{\bbeta\in\real^p} \normtwobig{\by-\bX{\bbeta}}$.
However, the term \textit{least squares} comes from the fact that what we are actually minimizing is the square of this norm: $\min_{\bbeta\in\real^p} \normtwobig{\by-\bX{\bbeta}}^2$, i.e., the least sum of squared errors.
This is equivalent to minimizing the norm itself because the square function is monotonically increasing for nonnegative values. That is:
$$
\argmin_{\bbeta\in\real^p} \normtwobig{\by-\bX{\bbeta}} \equiv \argmin_{\bbeta\in\real^p} \normtwobig{\by-\bX{\bbeta}}^2.
$$
In many computer programs using numerical methods, e.g., \textit{gradient descent and conjugate descent methods}, $f_1(\bbeta)\triangleq\normtwobig{\by-\bX{\bbeta}}^2$ is defined as the \textit{loss function} (or called \textit{objective function}, \textit{cost function}) rather than $\normtwobig{\by-\bX{\bbeta}}$ since the gradient of the former one can be derived easily:
$$
\nabla f_1(\bbeta) = 2(\bX^\top\by - \bX^\top\bX\bbeta).
$$
Due to the factor of 2 in the gradient, it is often more convenient to define the loss function as: 
$$
f_2(\bbeta)\triangleq \frac{1}{2}\normtwobig{\by-\bX{\bbeta}}^2,
$$
which simplifies expressions during optimization.
Another common variation is to normalize the loss by the number of samples $n$, resulting in: 
$$
f_3(\bbeta) \triangleq \frac{1}{n} \normtwo{\bX\bbeta-\by}^2
\qquad\text{or}\qquad f_4(\bbeta) \triangleq \frac{1}{2n} \normtwo{\bX\bbeta-\by}^2,
$$ 
where $n$ is the sample size.
There are several reasons why normalization by $n$ is beneficial:
\begin{itemize}
\item \textit{Mathematical convenience.} The term $\normtwo{\bX\bbeta-\by}^2$ represents the squared error summed over all $n$ samples. Without dividing by $n$, the loss value grows proportionally with the number of samples. This makes it harder to interpret or compare across datasets of different sizes.
\item Dividing by $n$ normalizes the loss to represent the \textit{average squared error per sample}, which is independent of the dataset size. This allows for more meaningful comparisons between models trained on datasets of varying sizes.
\item \textit{Numerical stability.} For large datasets, the unnormalized loss $\normtwo{\bX\bbeta-\by}^2$ can become very large, leading to potential numerical instability during optimization. Dividing by $n$ keeps the loss values in a reasonable range, improving numerical stability.
\item Using the normalized form often makes gradient-based optimization methods (e.g., stochastic gradient descent) more stable because the gradients are scaled appropriately.
\item \textit{Statistical estimation.} In statistics, the least squares problem can be viewed as maximizing the likelihood under the assumption that the errors are normally distributed. The normalized loss $\frac{1}{n} \normtwo{\bX\bbeta-\by}^2$ corresponds to the \textit{mean squared error (MSE)}, which is a commonly used metric in statistical modeling.
\item  \textit{Consistency.} Many other loss functions in machine learning (e.g., cross-entropy loss, hinge loss) are expressed as averages over the samples. Dividing by $n$ ensures consistency with these conventions, making it easier to switch between different loss functions or combine them in composite objectives.
\end{itemize}
In other words, dividing the least squares loss by $n$ serves to normalize the loss, making it interpretable, mathematically convenient, statistically meaningful, consistent with other loss functions, and numerically stable. While it does not affect the solution ${\bbeta}$ itself, it improves the overall process of optimization and interpretation.
In this book, we will use these forms interchangeably depending on context and convenience.

\paragrapharrow{Matrix LS problems.}
In many scenarios,  we may also consider the \textit{matrix least squares} (a.k.a., \textit{multiple-response least squares}) problem, which is formulated as:
\begin{equation}\label{equation:matrix_ls_raw}
\bB^* = \argmin_{\bB \in \real^{p \times q}} \normf{\bX\bB - \bY}^2 
=\argmin_{\bB \in \real^{p \times q}} \sum_{i=1}^{q} \normtwo{\bX\bbeta_i-\by_i}^2,
\end{equation}
where $\bX \in \real^{n \times p}$, $\bB=[\bbeta_1,\bbeta_2, \ldots,\bbeta_q]\in\real^{p\times q}$, and $\bY=[\by_1, \by_2, \ldots,\by_q] \in \real^{n \times q}$. 
This, apparently, can be solved using $q$ least squares problems \eqref{equation:ls_l2norm}. 
In other words, the problem is \textit{column-wise decomposable}.
However, when the number of responses $q$ is much larger than both the number of samples $n$ and the number of features $p$ (i.e., $q\gg \max\{n,p\}$), it becomes impractical to transmit or store the full response matrix $\bY$, especially in communication-constrained environments.
To address this issue, we introduce \textit{sketched least squares with quantized response}, which reduces the amount of data that needs to be communicated while still allowing for accurate estimation. This approach will be discussed in detail in Section~\ref{section:ske_qu_ls}.

\paragrapharrow{General insights.}
This LS problem is ubiquitous and has many well-known interpretations. 
From a statistical perspective, it provides the \textit{best linear unbiased estimator (BLUE)} under certain assumptions about the data-generating process; see Section~\ref{section:beta-blue}.
From a geometric perspective, the solution corresponds to the orthogonal projection of the response vector $\by$ onto the column space of the design matrix $\cspace(\bX)$; see Section~\ref{section:by-geometry-hat-matrix}.
Note that the latter interpretation is basically a statement about the data at hand, while the former interpretation is basically a statement about models and unseen data. This parallels the algorithmic-statistical approaches:
\begin{itemize}
\item \textbf{Algorithmic perspective} (see Chapter~\ref{chapter:ls_approx_num}). From an algorithmic perspective, the relevant question is: How long does it take to compute the least squares solution $\widehatbbeta$? The answer to this question is that it takes $\mathcalO(np^2)$ time. This can be accomplished using one of several numerical algorithms---with the Cholesky decomposition (which is good if $\bX$ has full column rank and is very well-conditioned); or with a variant of the QR decomposition (which is somewhat slower, but more numerically stable); or by computing the full SVD $\bX = \bU \bSigma \bV^\top$ (which is often, but certainly not always, overkill, but which can be easier to explain), and letting $\widehatbbeta = \bV \bSigma^{+} \bU^\top \by$. 
Although these methods differ significantly in terms of implementation and numerical behavior, asymptotically they all require roughly the same amount of time---on the order of a constant multiple of $np^2$.

\item \textbf{Statistical perspective} (see Chapters~\ref{sec:lr-gaussian-noise} and \ref{chapter:model_eva_sel}). From a statistical perspective, the relevant question is: Under what conditions is computing $\widehatbbeta$ the appropriate choice? 
The answer to this question is that this LS optimization is the right problem to solve when the relationship between the responses and predictors is roughly linear, when there are no small number of components that are particularly important or influential (called leverage points), and when the error processes generating the data are ``nice'' (in the sense that the errors have mean zero, constant variance, are uncorrelated, and are normally distributed; or when we have adequate sample size to rely on large sample theory). 
When these assumptions hold, the LS estimate $\widehatbbeta$ has strong theoretical guarantees. However, if the assumptions are violated---even slightly---the performance of LS can degrade significantly.
Thus, from a statistical perspective, a natural next question to ask is: What should one do when the assumptions underlying the use of LS methods are not satisfied or are only imperfectly satisfied?
\end{itemize}

In the remaining sections of this chapter, we will explore least squares solutions from several complementary perspectives: calculus, convex optimization, linear algebra, and geometry. 
That is, to obtain and to interpret the least squares solutions from each of these viewpoints.

\section{Least Squares in the Big Picture}\label{section:ls_big_pic}

We return to the basic problem $\min_{\bbeta}\normtwo{\by-\bX\bbeta}^2$, which
is at the core of our discussion. For the least squares problem, we consider the following questions:
\begin{itemize}
\item Q1: What is the least squares solution?
\item Q2: When can uniqueness of the least squares solution be claimed?
\item Q3: When can uniqueness of the least squares solution with minimum-norm be claimed?
\end{itemize}

To address these questions, we require the following lemma regarding the optimality condition for the least squares problem.
\begin{lemma}[Optimality condition of LS]\label{lemma:opt_cond_ls}
Given $ \bX \in \real^{n\times p} $ and $ \by \in \real^n $, let
$$
\sB \triangleq \{\bbeta\in \real^p\mid \normtwo{\by-\bX\bbeta}^2 =\min \}
$$
denote  the set of all least squares solutions. Then, $ \bbeta \in \sB $ if and only if the following orthogonality condition holds:
\begin{equation}\label{equation:sca_nor_ls_eq}
\bX^\top (\by - \bX\bbeta) = \bzero.~\footnote{Alternatively, we can also state that since $f(\bbeta)=\normtwo{\by-\bX\bbeta}^2$ is convex, the solution $\widehatbbeta$ must satisfy $\nabla f(\widehatbbeta) = 2\bX^\top(\by-\bX\bbeta)=\bzero$.
This is known as the \textit{first-order optimality condition} for (local) optima points. Note the proof of the first-order optimality condition for multivariate functions strongly relies on the first-order optimality conditions for one-dimensional functions, which is also known as  \textit{Fermat's theorem}. See Proposition~\ref{proposition:fermat_fist_opt}.}
\end{equation}
\end{lemma}
\begin{proof}[of Lemma~\ref{lemma:opt_cond_ls}]
Assume that $ \widehatbbeta $ satisfies $ \bX^\top \be = \bzero $, where $ \be \triangleq \by - \bX \widehatbbeta $. Then for any $ \bbeta \in \real^p $, we have
$ \widetilde{\be} = \by - \bX\bbeta = \be + \bX (\widehatbbeta - \bbeta) \triangleq \be + \bX\bd. $
From this we obtain
$ \widetilde{\be}^\top \widetilde{\be} = (\be + \bX\bd)^\top (\be + \bX\bd) = \be^\top \be + \normtwo{\bX\bd}^2$,
which is minimized when $\bd=\bzero$; that is, $ \bbeta = \widehatbbeta $. 

Conversely, suppose $ \bX^\top \be \triangleq \balpha \neq \bzero $. If $ \bbeta = \widehatbbeta + \gamma \balpha $, then
$ \widetilde{\be}  = \by - \bX\bbeta = \be - \gamma \bX\balpha $
and
$ \widetilde{\be}^\top \widetilde{\be} = \be^\top \be - 2 \gamma \balpha^\top \balpha + \gamma^2 (\bX\balpha)^\top \bX \balpha < \be^\top \be $
for sufficiently small $ \gamma > 0 $. Hence $ \widehatbbeta $ is not a least squares solution, which leads to a contradiction. Hence, $\bX^\top\be$ must be zero.  
\end{proof}

We now present a unified view of the least squares problem in the following theorem. The underlying ideas will become clearer as we proceed. This theorem answers the question Q1 introduced at the beginning of this section.
\begin{theoremHigh}[A unified view of least squares problems]\label{theorem:unif_ls}
Let $\bX\in\real^{n\times p}$ and $\by\in\real^n$. 
Then the least squares problem $f(\bbeta)=\normtwo{\by-\bX\bbeta}^2$ has a minimizer $\widehatbbeta\in\real^n$ if and only if there exists a vector $\balpha\in\real^p$ such that 
\begin{equation}\label{equation:unif_ls}
\widehatbbeta=\bX^+\by+(\bI-\bX^+\bX)\balpha,
\end{equation}
where $\bX^+$ denotes the Moore-Penrose  pseudo-inverse of $\bX$ (Section~\ref{section:pseudo-inverse}):
\begin{itemize}
\item This shows that the least squares has a \textbf{unique} minimizer of $\widehatbbeta=\bX^+\by$ only when $\bX^+$ is a left inverse of $\bX$ (Definition~\ref{definition:one_side_inverse}, and $\bX$ is left-invertible only when $\bX$ has full column rank by Lemma~\ref{theorem:one-sided-invertible}). The solution in \eqref{equation:ls_solu_pseudo} is a special case of this result.
\item The optimal value is $f(\widehatbbeta)=\by^\top(\bI-\bX\bX^+)\by$.
\item If $\balpha\neq \bzero$: $\normtwo{\bX^+\by}\leq \normtwo{\bX^+\by+(\bI-\bX^+\bX)\balpha}$.
\end{itemize}
\noindent
This means that {any vector $\bbeta$ that minimizes $f(\bbeta)$ must be in this form}, where:
\begin{itemize}
\item $\bX^+\by\perp \nspace(\bX)$  (by Theorem~\ref{theorem:pseudo-four-basis-space}) is the {particular solution} (the minimum-norm solution).
\item $(\bI - \bX^+ \bX)\balpha \in \nspace(\bX)$ (by Theorem~\ref{theorem:orthogonal-from-pseudo-inverse}) is the {homogeneous solution} that accounts for the freedom in $\bbeta$ coming from the null space of $\bX$.
\end{itemize}
\end{theoremHigh}
\begin{proof}[of Theorem~\ref{theorem:unif_ls}]
By Lemma~\ref{lemma:opt_cond_ls}, the solution $\widehatbbeta$ must satisfy  $\bX^\top(\by-\bX\bbeta)=\bzero$.
This equation means that the vector $\by-\bX\bbeta$ is {orthogonal to the column space of $\bX$}, i.e., the error $\by-\bX\bbeta$ lies in the {null space of $\bX^\top$}:
\begin{equation}
\by-\bX\bbeta \perp \cspace(\bX)
\qquad\text{and} \qquad 
\by-\bX\bbeta\in \nspace(\bX^\top).
\end{equation}

To solve for $\bbeta$, we recognize that the equation $\bX\bbeta =\by$ may not always have an exact solution (i.e., when $\by\notin \cspace(\bX)$). Instead, we seek the {minimum-norm solution} that minimizes $\normtwo{\by - \bX\bbeta}$.
The {pseudo-inverse} $\bX^+$ provides the best possible solution by giving the unique {minimum-norm least squares solution} (see the argument in the sequel):
\begin{equation}
\bbeta_{\text{particular}} = \bX^+\by.
\end{equation}
To see $\bbeta_{\text{particular}}$ satisfies \eqref{equation:sca_nor_ls_eq}, we have 
$\bX^\top\bX\bbeta_{\text{particular}} =\bX^\top\bX\bX^+\by =\bX^\top\by$, where we used the fact that $\bX^\top\bX\bX^+=\bX^\top$ (Lemma~\ref{lemma:xtxxpllus_pseudo}).

However, this is just {one possible solution}. 
To see the full solution set, we consider two cases.
\paragraph{Case 1: $\by\in\cspace(\bX)$.}
For this case, the solution in \eqref{equation:unif_ls} is obvious by the properties ($\bX(\bI-\bX^+\bX)\balpha=\bzero$) and uniqueness of the pseudo-inverse of a matrix (Lemma~\ref{lemma:uniqueness-of-pseudo-inverse}).

\paragraph{Case 2: $\by\notin \cspace(\bX)$.} Let $\be\triangleq \by-\bX\bbeta_{\text{particular}} = \by-\bX\bX^+\by$. Since we can looking for solutions $\widetildebbeta$ that satisfy $\normtwobig{\by-\bX\widetildebbeta}^2 = \normtwo{\be}^2$. Two possible scenarios are:
$$
\bX\widetildebbeta-\by = \be
\qquad \text{and}\qquad 
\by-\bX\widetildebbeta = \be.
$$
The former scenario implies that $\be= \by-\bX\bbeta_{\text{particular}} = \bX\widetildebbeta-\by$, which leads to $\by = \bX\frac{\widetildebbeta+\bbeta_{\text{particular}}}{2}$, i.e., $\by\in\cspace(\bX)$, matching Case 1. 
Therefore, only the second scenario applies.
Since $\bbeta$ is in $\real^p$, the full space of solutions consists of the {particular solution} $\bbeta_{\text{particular}} = \bX^+\by$ plus any vector in the {null space of $\bX$} (i.e., any $\bbeta_{\text{null}}$ that satisfies $\bX\bbeta_{\text{null}} = \bzero$, $\bbeta_{\text{null}}\in	 \nspace(\bX)\triangleq \{\bz \mid \bX\bz = \bzero\}$).
Thus, the {general solution} for $\bbeta$ must then take the form:
\begin{equation}
	\bbeta = \bX^+\by + \bbeta_{\text{null}}, \qquad \text{with}\quad \bbeta_{\text{null}}\in\nspace(\bX).
\end{equation}
Since $\bX-\bX\bX^+\bX = \bzero$ by \eqref{equation:pseudi-four-equations}, the projection onto the {null space of $\bX$} is given by:
\begin{equation}
	\bbeta_{\text{null}} = (\bI - \bX^+ \bX)\balpha, \quad \text{for some }\balpha \in \real^p.
\end{equation}
This follows because $(\bI - \bX^+ \bX)$ is a {projection matrix} onto $\nspace(\bX)$ (see Theorem~\ref{theorem:orthogonal-from-pseudo-inverse} and Section~\ref{section:by-geometry-hat-matrix} for more details).

On the other hand, since $\bbeta_{\text{null}} \in \nspace(\bX)$ and $\cspace(\bX^+) \equiv \cspace(\bX^\top)$ (see Theorem~\ref{theorem:pseudo-four-basis-space}, Theorem~\ref{theorem:fundamental-linear-algebra}), it follows that
$$
\normtwo{\bX^+\by+(\bI-\bX^+\bX)\balpha}^2 = \normtwo{\bX^+\by}^2 + \normtwo{(\bI-\bX^+\bX)\balpha}^2.
$$
This shows that $\normtwo{\bX^+\by}\leq \normtwo{\bX^+\by+(\bI-\bX^+\bX)\balpha}$ if $\balpha\neq \bzero$ and completes the proof.
\end{proof}

\index{Normal equation}
\begin{definition}[Normal equation]\label{definition:normal_equ}
We can express the zero gradient of $\normtwo{\by-\bX\bbeta}^2$ w.r.t. $\bbeta$ as $\bX^\top\bX \widehatbbeta = \bX^\top\by$. The equation is also known as the \textit{normal equation}. 
Gauss developed an elimination method for solving the normal equation that uses pivots chosen from the diagonal \citep{stewart1995gauss}. Then all
reduced matrices are symmetric, and the storage and number of needed operations are reduced
by half. Later, the preferred way to implement this elimination process became the Cholesky decomposition.
\end{definition}

The above analysis shows that $\widehatbbeta$ is a least squares solution if and only if the residual $\be\triangleq \by - \bX\widehatbbeta$ is perpendicular to $\cspace(\bX)$ from the normal equation $\bX^\top (\by-\bX\bbeta)\triangleq \bX^\top\be=\bzero$:
\begin{equation}
\bX\widehatbbeta \in \cspace(\bX)
\qquad \text{and}\qquad 
\be=\by-\bX\widehatbbeta  \perp \cspace(\bX).
\end{equation}
For this reason, the residual $\be$ is sometimes denoted as $\by^{\perp}\triangleq \be =\by-\bX\widehatbbeta$.
Based on this, we can now state the following equivalent characterizations of a least squares solution.

\begin{corollary}[Least squares solution]\label{corollary:ls_equiv}
The following statements are equivalent:
\begin{enumerate}[(i)]
\item $ \widehatbbeta $ solves the least squares problem $\min_{\bbeta} \normtwo{\by-\bX\bbeta}$.
\item $ \widehatbbeta $ satisfies the normal equation $\bX^\top \bX \widehatbbeta = \bX^\top\by$.
\item The residual $ \be=\by-\bX\widehatbbeta $ is orthogonal to $\cspace(\bX)$.
\end{enumerate}
\end{corollary}

Since the least squares estimate is widely used in regression problems, it is important to understand how it generates predictions. After obtaining the least squares estimate  $\widehatbbeta$ from the data matrix $\bX$ and the response vector $\by$, the predicted value of $\by$ is given by 
$$
\widehatby = \bX\widehatbbeta.
$$
For a new input vector $\bx_{\new}$, the corresponding \textit{prediction} (also referred to as the ``\textit{curve-fit}") is simply computed as
$$
\widehat{y}_{\new} = \widehatbbeta^\top\bx_{\new}.
$$
Theorem~\ref{theorem:unif_ls} shows that $\bX^+\by$ is a minimum-norm solution for the least squares problem in general. 
Based on the properties of the pseudo-inverse, we can distinguish several important cases:
\begin{itemize}
\item \textbf{Large-sample least squares problem.} If $n>p=\rank(\bX)$, then the pseudo-inverse of $\bX$ is given by  $\bX^+ = (\bX^\top\bX)^{-1}\bX^\top$~\footnote{$\bX^\top\bX$ is nonsingular by Lemma~\ref{lemma:rank-of-ata-x}.}. This shows 
$$
\widehatbbeta=\bX^+\by+(\bI-\bX^+\bX)\balpha = (\bX^\top\bX)^{-1}\bX^\top\by.
$$
That is, the least squares solution is \textbf{unique} and answers the question Q2 introduced at the beginning of this section. 
This scenario---where there are more observations than predictors---is often referred to as the \textit{large-sample LS problem}. It will be discussed further in
Sections~\ref{section:ols_calculus}, \ref{section:ls_conv}, \ref{section:ls-fundation-theorem}, \ref{section:by-geometry-hat-matrix}, \ref{section:application-ls-qr}, and Chapter~\ref{sec:lr-gaussian-noise}.

\item \textbf{Rank-deficient least squares problem.} If $n>p>\rank(\bX)$, then the matrix $\bX$ does not have full rank and $\bX=\bV\bSigma^+\bU^\top\by$ is a minimizer, where $\bX=\bU\bSigma\bV$ denotes the SVD of $\bX$. This rank-deficient LS problem will be further discussed in Theorem~\ref{theorem:rank_def_ls_prop}, and Sections~\ref{section:ls-via-svd} and \ref{section:utv_ls}.

\index{High-dimensional LS}
\item \textbf{High-dimensional least squares problem.} If $p\geq n=\rank(\bX)$, then the pseudo-inverse of $\bX$ is given by $\bX^+ = \bX^\top(\bX\bX^\top)^{-1}$~\footnote{$\bX\bX^\top$ is nonsingular by a similar argument in Lemma~\ref{lemma:rank-of-ata-x}.}. 
In this case, the general least squares solution becomes
$$
\widehatbbeta=\bX^+\by+(\bI-\bX^+\bX)\balpha = \bX^\top(\bX\bX^\top)^{-1}\by+(\bI-\bX^\top(\bX\bX^\top)^{-1}\bX)\balpha.
$$
Here, the minimum-norm solution is $\widehatbbeta_{\text{mn}} \triangleq \bX^\top(\bX\bX^\top)^{-1}\by$. 
This high-dimensional LS problem will be further discussed in Section~\ref{section:ls-via-svd}.
\end{itemize}

\paragrapharrow{Matrix least squares problems.}
When the response variable $\by$ is extended from a vector to a matrix $\bY$, the problem becomes what is known as a \textit{matrix least squares problem}. 
A direct consequence of Theorem~\ref{theorem:unif_ls} is the following result for such problems.
\begin{corollary}[Matrix least squares]\label{corollary:matrix_ls}
Let $\bX\in\real^{n\times p}$.
\begin{itemize}
\item Given further $\bY\in\real^{p\times q}$, then the least squares problem $f_1(\bB) = \normf{\bX\bB-\bY}^2$ has a minimizer $\widehatbB=\bX^+\bY\in\real^{p\times q}$. 
\item Given further  $\bY\in\real^{q\times p}$, then the least squares problem $f_2(\bB) = \normf{\bB\bX-\bY}^2$ has a minimizer $\widehatbB=\bY\bX^+\in\real^{q\times n}$.
\end{itemize}
\end{corollary}

\subsection*{Rank-Deficient Least Squares Problems}
If $ r = \rank(\bX) < p $, then $ \bX $ has a null space of dimension $ p - r > 0 $. Then the problem $\min_{\bbeta} \normtwo{\bX\bbeta - \by}$ is \textit{rank-deficient}, and its solution is not unique. If $\widehatbbeta$ is a particular least squares solution, then the set of all least squares solutions is $ \sB = \{ \bbeta = \widehatbbeta + \balpha \mid \balpha \in \nspace(\bX) \} $. In this case we can seek the least squares solution of minimum-norm $\normtwo{\bbeta}$, i.e., solve
\begin{equation}\label{equation:min_norm_prob}
	\min_{\bbeta \in \sB}\normtwo{\bbeta}, \qquad \sB = \{ \bbeta \in \real^p \mid \normtwo{\by - \bX\bbeta} = \min \}.
\end{equation}
This solution is always unique (question Q3 introduced at the beginning of this section); see Theorem~\ref{theorem:rank_def_ls_prop}.
For the least squares problem, the set of all minimizers $\sB $ is convex. 
To see this, let $\bbeta_1, \bbeta_2 \in \sB$ and $\lambda \in [0,1]$. Then we have:
$$
\normtwo{\by-\bX(\lambda\bbeta_1 + (1-\lambda)\bbeta_2) } \leq \lambda\normtwo{\by-\bX\bbeta_1} +(1-\lambda)\normtwo{\by-\bX\bbeta_2} = \mathop{\min}_{\bbeta\in \real^p} \normtwo{\by-\bX\bbeta}. 
$$
Therefore, the convex combination $\lambda\bbeta_1 + (1-\lambda)\bbeta_2$ also belongs to $\sB$, which confirms that $\sB$ is a convex set.

\subsection*{Consistent System}
\index{Linear system}
\index{Consistent linear system}
Consider the linear system $\bX\bbeta = \by$, where $\bX\in\real^{n\times p}$.
If there is at least one solution to this system, it is called \textit{consistent}; otherwise, it is called \textit{inconsistent}. It can be shown that the system is consistent if and only if $\rank([\bX,\by])=\rank(\bX)$, i.e., the \textit{augmented matrix} $[\bX,\by]$ has the same rank as the observed coefficient matrix $\bX$.
This condition also implies that $\by$ lies in the column space of $\bX$.
For consistent systems, we can distinguish the following two cases:
\begin{itemize}
\item When $\bX$ has full column rank $p$, the linear system has a  \textbf{unique} solution: $\widehatbbeta = (\bX^\top\bX)^{-1}\bX^\top \by$. Refer to its description in the left inverse (Theorem~\ref{theorem:unique-linear-system-solution}).
\item When $\bX$ has full row rank $n$, the linear system has at least one solution: $\widehatbbeta = \bX_R^{-1}\by$, where $\bX_R^{-1}$ is a right inverse of $\bX$. Refer to its description in the right inverse (Theorem~\ref{theorem:always-have-solution-right-inverse}).
\end{itemize}

If the system $\bX\bbeta = \by$ is consistent, then the minimum-norm solution satisfies the normal equation of second kind; see Theorem~\ref{theorem:rank_def_ls_prop} below for more details:
\begin{equation}\label{equation:consis_minimunorm}
\text{(NE2)}:\qquad 
\bbeta = \bX^\top \bgamma
\quad\implies\quad 
 \quad \bX\bX^\top \bgamma = \by. 
\end{equation}
If $\rank(\bX) = n$, then $\bX\bX^\top$ is nonsingular, and the solution to \eqref{equation:consis_minimunorm} (i.e., the minimum-norm solution) is \textbf{unique}:
\begin{equation}
\widehatbbeta = \bX^\top (\bX\bX^\top)^{-1}\by\equiv \bX^+\by.
\end{equation}

\subsection*{Uniqueness of LS Problems}

From the result in \eqref{equation:consis_minimunorm} and Corollary~\ref{corollary:ls_equiv}, we obtain the following characterization of a solution to the least squares problem \eqref{equation:min_norm_prob}. It includes both the over- and underdetermined cases.

\index{Minimum-norm solution}
\begin{theoremHigh}[Minimum-norm solution]\label{theorem:rank_def_ls_prop}
Let $\bX\in\real^{n\times p}$ and $\by\in\real^n$.
And let $ \bbeta $ be a solution of the problem $\min_{\bbeta} \normtwo{\bX\bbeta - \by}$. Then $ \bbeta $ is a \textbf{unique} least squares solution of minimum-norm if and only if $ \bbeta \perp \nspace(\bX) $ or, equivalently, $ \bbeta = \bX^\top \bgamma $, $ \bgamma \in \real^n $.~\footnote{And Theorem~\ref{theorem:unif_ls} shows this minimum-norm solution is $\widehatbbeta_{\text{mn}}=\bX^+\by$, which satisfies $\widehatbbeta_{\text{mn}} \perp\nspace(\bX)$ since $\cspace(\bX^+) \equiv \cspace(\bX^\top)$ by Theorem~\ref{theorem:pseudo-four-basis-space}.}
That is, for $\bX\in\real^{n\times p}$ of any dimension and rank, the least squares solution of minimum norm $\normtwo{\bbeta}$ is \textbf{unique} and characterized by the conditions
\begin{equation}
	\be \triangleq \by - \bX\bbeta \perp \cspace(\bX) 
	\qquad \text{and} \qquad 
	\bbeta \perp \nspace(\bX).
\end{equation}
\end{theoremHigh}

\begin{proof}[of Theorem~\ref{theorem:rank_def_ls_prop}]
Assume $\bbeta\perp \nspace(\bX)$.
Let $\widetildebbeta$ be any least squares solution, and set $ \widetildebbeta = \bbeta + \balpha $, where $ \balpha \in \nspace(\bX) $. Then $ \bX \balpha = \bzero $, so $ \be =  \by - \bX\widetildebbeta = \by - \bX\bbeta $, and $\bbeta$ is also a least squares solution. Since $\bbeta\perp\nspace(\bX)$ and $\balpha\in\nspace(\bX)$, by the Pythagorean theorem, $\normtwobig{\widetildebbeta}^2 = \normtwo{\bbeta}^2  + \normtwo{\balpha}^2$, which is minimized when $ \balpha =\bzero $.
The reverse claim follows from the normal equation. 
\end{proof}

\subsection*{Augmented LS Problem}

From the normal equation $\bX^\top\bX\bbeta=\bX^\top\by$, we show that the error component $\be\triangleq \by-\bX\bbeta$ satisfies the orthogonality condition $\bX^\top\be=\bzero$. These forms a symmetric augmented system of $n+p$ equations:
\begin{equation}\label{equation:sys_aug_sys_v0}
\begin{bmatrix}
\bI & \bX \\
\bX^\top & \bzero
\end{bmatrix}
\begin{bmatrix}
\be \\
\bbeta
\end{bmatrix}
=
\begin{bmatrix}
\by \\
\bzero
\end{bmatrix},
\quad \by \in \real^n.
\end{equation}
Apparently, this augmented system is a special case of the following augmented system, which we call the \textit{augmented LS (AuLS) problem}:
\begin{equation}\label{equation:sys_aug_sys}
(\text{AuLS}):\qquad 
\begin{bmatrix}
\bI & \bX \\
\bX^\top & \bzero
\end{bmatrix}
\begin{bmatrix}
\balpha \\
\bbeta
\end{bmatrix}
=
\begin{bmatrix}
\by \\
\bz
\end{bmatrix},
\quad \by \in \real^n, \quad \bz \in \real^p,
\end{equation}

On the other hand, we also consider the high-dimensional problem (consistent underdetermined linear system), of which the minimum-norm solution is 
$$
(\text{MN}): \quad \min_{\balpha} \normtwo{\balpha}^2 \ \text{s.t.}\ \bX^\top\balpha =\bz
\quad\implies \quad
\balpha =\bX\bgamma 
\quad\implies\quad
\bX^\top\bX\bgamma =\bz, 
$$ 
~\footnote{Previously, we assume $\bX\in\real^{n\times p}$ with $\rank(\bX)=n$ for this consistent linear system. To make it consistent with the context, we use $\bX^\top$ with $\rank(\bX)=p$ here.}
where $\bX^\top\bX\bgamma =\bz$ is called the normal equation of the second kind; see \eqref{equation:consis_minimunorm}.
Therefore, the minimum-norm solution is also a special case of the AuLS system by setting $\bbeta=\bzero$ and $\by=\balpha$.

Therefore, both the standard least squares and minimum-norm problems are special cases of the augmented LS problem.
Such AuLS systems represent the equilibrium of a physical system and occur in many applications; see \citet{strang1988framework, bjorck2024numerical} for more details.
The system is nonsingular if and and only if  $\rank(\bX)=p$, and its inverse is 
\begin{equation}\label{equation:inv_aug_mat}
\begin{bmatrix}
\bI & \bX \\
\bX^\top & \bzero
\end{bmatrix}
=
\begin{bmatrix}
\bI - \bX(\bX^\top\bX)^{-1}\bX^\top & \bX(\bX^\top\bX)^{-1}\\
(\bX^\top\bX)^{-1}\bX^\top & -(\bX^\top\bX)^{-1}
\end{bmatrix},
\quad \text{when }\rank(\bX)=p.
\end{equation}
Note that  $\bH \triangleq \bX(\bX^\top \bX)^{-1} \bX^\top$ is the orthogonal projector onto $\cspace(\bX)$; see Section~\ref{section:by-geometry-hat-matrix}.

\begin{theoremHigh}
If $\rank(\bX) = p$, then the augmented system \eqref{equation:sys_aug_sys} has a unique solution that solves the primal and dual least squares problems,
\begin{align}
&\min_{\bbeta \in \real^p} \frac{1}{2} \normtwo{\by - \bX\bbeta}^2 + \bz^\top \bbeta, \label{equation:augsys_eq1}\\
&\min_{\balpha \in \real^n} \frac{1}{2} \normtwo{\balpha - \by}^2 \quad \text{s.t.} \quad \bX^\top \balpha = \bz. \label{equation:augsys_eq2}
\end{align}
\end{theoremHigh}

\begin{proof}
Differentiating \eqref{equation:augsys_eq1} gives $\bX^\top(\by - \bX\bbeta) = \bz$, which with $\balpha = \by - \bX\bbeta$ is the augmented system \eqref{equation:sys_aug_sys}. This system is also obtained by differentiating the Lagrangian
$ L(\bbeta, \balpha) = \frac{1}{2} (\balpha - \by)^\top (\balpha - \by) + \bbeta^\top (\bX^\top \balpha - \bz) $
for \eqref{equation:augsys_eq2} and equating to zero, where $\bbeta$ is the vector of Lagrange multipliers.
\end{proof}

\subsection*{Special Least Squares Systems}

If the columns $\bX[:,j]$ of $\bX$ are mutually orthogonal, then least squares problems simplify considerably. The reason is that $\bX^\top \bX$ is diagonal and then so is $(\bX^\top \bX)^{-1}$ in the full rank case. The result is that
$$
\widehatbeta_j = \frac{\sum_{i=1}^n x_{ij} y_i}{\sum_{i=1}^n x_{ij}^2} 
= \frac{\bX[:,j]^\top \by}{\normtwo{\bX[:,j]}^2}, \qquad \forall \, j \in\{1,2 \ldots, p\}.
$$
This means the coefficients can be computed independently, one at a time, without needing to solve a system of equations.
In the even more special case that each of $\bX[:,j]$ is a unit vector (or the columns $\bX$ are mutually orthogonormal), then $\widehatbeta_j = \bX[:,j]^\top \by$.

Orthogonal predictors bring great simplification. The cost of computation is only $O(np)$. The variance of $\widehatbbeta$ is $\sigma^2 \diag(1/\normtwo{\bX[:,j]}^2)$ (see Chapter~\ref{sec:lr-gaussian-noise} for more details) so the components of $\widehatbbeta$ are uncorrelated. In the Gaussian case, the $\widehatbeta_j$ are statistically independent in addition to the computational independence noted above.

\subsection*{Leave One Out Formula}
In this subsection we explore what happens to a (full-rank) least squares model when one data point (a row in $\bX$) is added or removed. We begin with the \textit{Sherman-Morrison formula}. Suppose that $\bA$ is an invertible $n \times n$ matrix, and let  $\bu$ and $\bv$ be $n$-dimensional vectors such that $1 + \bv^\top\bA^{-1}\bu \neq 0$. Then
\begin{equation}\label{equation:she_morr_for}
(\bA + \bu\bv^\top)^{-1} = \bA^{-1} - \frac{\bA^{-1}\bu\bv^\top\bA^{-1}}{1 + \bv^\top\bA^{-1}\bu}.
\end{equation}
This can be proved by multiplying the right-hand side of the equation by $\bA + \bu\bv^\top$ and checking that the product equals the identity matrix. The condition $1 + \bv^\top\bA^{-1}\bu \neq 0$ ensures that the updated matrix  $\bA + \bu\bv^\top$ remains invertible if $\bA$ is nonsingular.

Now suppose we delete the $i$-th observation from the data $\bX$. 
Then $\bX^\top \bX = \sum_{i=1}^n \bx_i \bx_i^\top$ is replaced by $(\bX^\top \bX)_{(-i)} = \bX^\top \bX - \bx_i \bx_i^\top$, where $\bx_i\in\real^p$ is a column vector representing  the $i$-th row of $\bX$, using a subscript of $(-i)$ to denote the removal of the  $i$-th observation. We can fit this into \eqref{equation:she_morr_for} by taking $\bu = \bx_i$ and $\bv = -\bx_i$. Then,
$$
\begin{aligned}
(\bX^\top \bX)^{-1}_{(-i)} 
&= (\bX^\top \bX)^{-1} + \frac{(\bX^\top \bX)^{-1}\bx_i \bx_i^\top(\bX^\top \bX)^{-1}}{1 - \bx_i^\top(\bX^\top \bX)^{-1}\bx_i}\\
&\triangleq (\bX^\top \bX)^{-1} + \frac{(\bX^\top \bX)^{-1}\bx_i \bx_i^\top(\bX^\top \bX)^{-1}}{1 - h_{ii}},
\end{aligned}
$$
where we let $\bH\triangleq \bX(\bX^\top\bX)^{-1}\bX^\top$, and $h_{ii}\triangleq \bx_i^\top(\bX^\top \bX)^{-1}\bx_i$. 
We also find that $(\bX^\top \by)_{(-i)} = \bX^\top \by - \bx_i y_i$. 
Therefore, the leave-one-out update for the least squares solution $\widehatbbeta$ becomes
$$
\begin{aligned}
\widehatbbeta_{(-i)} 
&= \Big( (\bX^\top \bX)^{-1} + \frac{(\bX^\top \bX)^{-1} \bx_i \bx_i^\top (\bX^\top \bX)^{-1}}{1 - h_{ii}} \Big) 
\left( \bX^\top \by - \bx_i y_i \right)\\
&= \widehatbbeta - (\bX^\top \bX)^{-1} \bx_i y_i + \frac{(\bX^\top \bX)^{-1} \bx_i \bx_i^\top \widehatbbeta}{1 - h_{ii}} - \frac{(\bX^\top \bX)^{-1} \bx_i h_{ii} y_i}{1 - h_{ii}} \\
&= \widehatbbeta + \frac{(\bX^\top \bX)^{-1} \bx_i \bx_i^\top \widehatbbeta}{1 - h_{ii}} - \frac{(\bX^\top \bX)^{-1} \bx_i y_i}{1 - h_{ii}} 
= \widehatbbeta - \frac{(\bX^\top \bX)^{-1} \bx_i (y_i - \widehaty_i)}{1 - h_{ii}},
\end{aligned}
$$
where $\widehaty_i \triangleq \bx_i^\top \widehatbbeta$.
Thus, the prediction for $y_i$ when $(\bx_i, y_i)$ is removed from the least squares fit is
$$
\widehaty_{i,(-i)} \triangleq \bx_i^\top \widehatbbeta_{(-i)} = \widehaty_i - \frac{h_{ii} (y_i - \widehaty_i)}{1 - h_{ii}}.
$$
Multiplying both sides by $1 - h_{ii}$ and rearranging gives:
\begin{equation}\label{equation:yhat_wei_yyhatloo}
\widehaty_i = h_{ii} y_i + (1 - h_{ii}) \widehaty_{i,(-i)}.
\end{equation}
Equation~\eqref{equation:yhat_wei_yyhatloo} has an important interpretation. The least squares fit $\widehaty_i$ is a weighted combination of $y_i$ itself and the least squares prediction we would have made for it, had it been left out of the fitting. The larger $h_{ii}$ is, the more that $\widehaty_i$ depends on $y_i$. It also means that if we want to compute a ``leave one out" residual $y_i - \widehaty_{i,(-i)}$, we don’t have to actually take $(\bx_i, y_i)$ out of the data and rerun the estimate. We can instead use
\begin{equation}
y_i - \widehaty_{i,(-i)} = \frac{y_i - \widehaty_i}{1 - h_{ii}}.
\end{equation}
This analysis will be important for the diagnostics for linear models; see Section~\ref{section:gau_diag}.

\section{OLS in Calculus}\index{OLS by calculus}\label{section:ols_calculus}
\index{Fermat's theorem}
\textit{Fermat's theorem}, also known as \textit{Fermat's theorem on stationary points}, is a fundamental result in calculus and mathematical optimization. It provides a necessary condition for a function to have a local optimum (either a local maximum or a local minimum) at a point inside the domain of the function.
For a univariate function, Fermat's theorem states the optimality condition for a optimal point that lies in the interior of a set, i.e., a one-dimensional constrained optimization problem.
\begin{proposition}[Fermat's theorem: necessary condition for univariate functions]\label{proposition:fermat_theorem}
	Let $f: (a,b)\rightarrow \real$ be a univariate differentiable function defined over an interval ($a, b$). 
	If a point $\widehattheta\in(a,b)$ (i.e., $\widehattheta\in\interior([a,b])$) is a local maximum or minimum, then $f^\prime(\widehattheta)=0$. 
\end{proposition}

In other words, if a function $f$  has a local maximum or minimum at a point $\widehattheta$, and $f$ is differentiable at that point, then the slope of the tangent line at $\widehattheta$ must be zero; that is, the derivative of $\widehattheta$ is zero.

It is important to note that this condition is necessary but \textbf{not} sufficient for $\widehattheta$ to be a local extremum. There are cases where the derivative is zero, but the point is neither a maximum nor a minimum---for example, at an inflection point.

Additionally, Fermat's theorem does \textbf{not} apply to boundary points of the domain of $f$, or to points where $f$ is not differentiable.

Most objective functions, especially those with multiple local minima, also contain local maxima and other critical points that satisfy the necessary condition given by Fermat's theorem. To distinguish true local minima from these irrelevant or non-optimal critical points, we rely on additional theorems and conditions, which help us better characterize and classify such points.

We now state the first-order necessary condition for a local minimum point in multivariate optimization.
\begin{proposition}[First-order necessary condition for a  minimum point]\label{proposition:fermat_fist_opt}
Let $f: \real^p \rightarrow \real$ be a  differentiable function. If $\widehatbtheta$ is a (local) minimizer for $f$, then
$$
\nabla f(\widehatbtheta) = \bzero.
$$
This is known as a \textit{stationary point} of $f$.
\end{proposition}
\begin{proof}[of Proposition~\ref{proposition:fermat_fist_opt}]
Let $i \in \{1, 2, \ldots, p\}$, and define the one-dimensional function $g(\mu) = f(\widehatbtheta + \mu \be_i)$. Note that $g$ is differentiable at $\mu = 0$ and that $g'(0) = \frac{\partial f}{\partial x_i}(\widehatbtheta)$. Since $\widehatbtheta$ is a local minimum point of $f$, it follows that $\mu = 0$ is a local minimum point of $g$, which immediately implies that $g'(0) = 0$ by Proposition~\ref{proposition:fermat_theorem}. This equality is exactly the same as $\frac{\partial f}{\partial x_i}(\widehatbtheta) = 0$. Since this holds for any $i \in \{1, 2, \ldots, p\}$, we obtain $\nabla f(\widehatbtheta) = \bzero$.
\end{proof}

When objective $\normtwo{\by-\bX\bbeta}^2$ is differentiable, and the parameter space of $\bbeta$ includes  the entire space $\real^p$ (so that the minimum occurs in the interior of the domain), the least squares estimate must occur at a point where the gradient of the function is zero. We thus come into the following theorem.

\index{Fermat's theorem}
\begin{theoremHigh}[Least squares by calculus]\label{theorem:ols}
Assume the observed data matrix $\bX \in \real^{n\times p}$ is fixed and has full rank (i.e., the columns of $\bX$ are linearly independent) with $n\geq p$. Consider the overdetermined system $\by = \bX\bbeta$, the least squares solution by calculus via setting the derivative in every direction of $\normtwo{\by-\bX\bbeta}^2$ to be zero is $\widehatbbeta = (\bX^\top\bX)^{-1}\bX^\top\by$. 
The value $\widehatbbeta = (\bX^\top\bX)^{-1}\bX^\top\by\equiv \bX^+\by$ is known as the \textit{ordinary least squares (OLS)} estimate or simply least squares (LS) estimate of $\bbeta$.
\end{theoremHigh}

\begin{proof}[of Theorem~\ref{theorem:ols}]
From Proposition~\ref{proposition:fermat_fist_opt}, a function $f(\bbeta)$ attains a minimum  at a point $\widehatbbeta$ if its gradient $\nabla f(\bbeta)=\bzero$. 
In our case, the objective function is $\normtwo{\by-\bX\bbeta}^2$, whose gradient is $2\bX^\top\bX\bbeta -2\bX^\top\by$. 
The condition  $2\bX^\top\bX\bbeta -2\bX^\top\by=\bzero $ thus aligns with the normal equation (Definition~\ref{definition:normal_equ}).
The matrix $\bX^\top\bX$ is invertible since we assume $\bX$ is fixed and has full rank (Lemma~\ref{lemma:rank-of-ata-x}). So the OLS solution of $\bbeta$ is $\widehatbbeta = (\bX^\top\bX)^{-1}\bX^\top\by$, from which the result follows.
\end{proof}

\index{Convex functions}
\begin{figure}[h!]
	\centering  
	\vspace{-0.35cm} 
	\subfigtopskip=2pt 
	\subfigbottomskip=2pt 
	\subfigcapskip=-5pt  
	\subfigure[A convex function.]{\label{fig:convex-1}
		\includegraphics[width=0.31\linewidth]{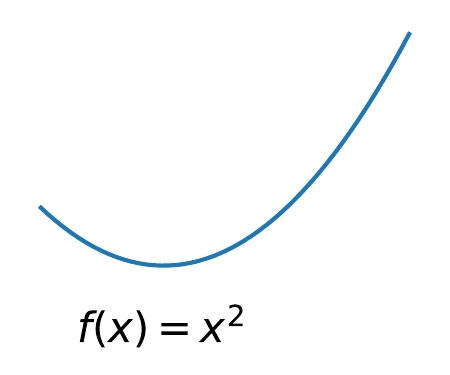}}
	\subfigure[A concave function.]{\label{fig:convex-2}
		\includegraphics[width=0.31\linewidth]{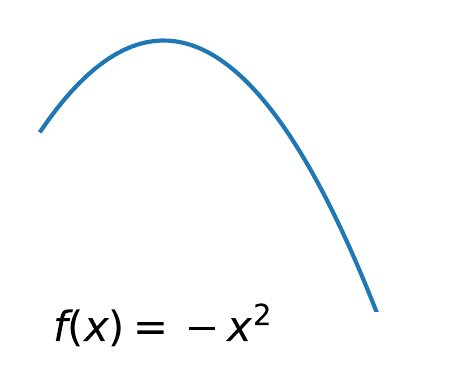}}
	\subfigure[A random function.]{\label{fig:convex-3}
		\includegraphics[width=0.31\linewidth]{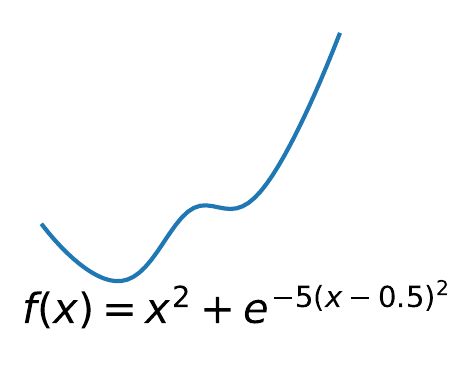}}
	\caption{Three functions.}
	\label{fig:convex-concave-none}
\end{figure}

However, we cannot be certain whether the least squares estimate obtained in Theorem~\ref{theorem:ols} corresponds to a minimum, maximum, or neither. An illustrative example is shown in Figure~\ref{fig:convex-concave-none}. 
Our current understanding only confirms that the function $\normtwo{\by-\bX\bbeta}^2$ has a single critical point (a root of its gradient), which is a \textbf{necessary} condition for a minimum---but not necessarily a \textbf{sufficient} one.
Further clarification on this issue is provided in the following remark. Alternatively, a more general explanation can be given using convex analysis (see Section~\ref{section:ls_conv}).
\begin{remark}[Verification of least squares solution]

Why does a zero gradient imply the least mean squared error?
We avoid discussing convexity (as we will shortly see) in detail here in order to keep things simple. However, we directly verify that the OLS solution indeed minimizes the sum of squared errors.
For any $\bbeta \neq \widehatbbeta$, consider the following expansion:
\begin{equation}
\begin{aligned}
\normtwo{\by - \bX\bbeta}^2 
&= \normtwobig{\by - \bX\widehatbbeta + \bX\widehatbbeta - \bX\bbeta}^2
= \normtwobig{\by-\bX\widehatbbeta + \bX (\widehatbbeta - \bbeta)}^2 \\
&=\normtwobig{\by-\bX\widehatbbeta}^2 + \normtwobig{\bX(\widehatbbeta - \bbeta)}^2 + 2\big(\bX(\widehatbbeta - \bbeta)\big)^\top(\by-\bX\widehatbbeta) \\ 
&=\normtwobig{\by-\bX\widehatbbeta}^2 + \normtwobig{\bX(\widehatbbeta - \bbeta)}^2 + 2(\widehatbbeta - \bbeta)^\top(\bX^\top\by - \bX^\top\bX\widehatbbeta), \nonumber
\end{aligned} 
\end{equation}
where the third term is zero due to the normal equation, and it also follows that  $\normtwobig{\bX(\widehatbbeta - \bbeta)}^2 \geq 0$. Therefore,
\begin{equation}
\normtwobig{\by - \bX\bbeta }^2 \geq \normtwobig{\by-\bX\widehatbbeta}^2. \nonumber
\end{equation}
This shows that the OLS estimate corresponds to a minimum, rather than a maximum or saddle point, using a calculus-based argument.
In fact, the condition arising from the least squares estimate is also referred to as the \textit{sufficiency of stationarity under convexity}. When $\bbeta$ is defined across the entire space $\real^p$, this condition is alternatively recognized as the \textit{necessity and sufficiency of stationarity under convexity}.
\end{remark}

A natural question arises: Why does the normal equation magically produce a solution for $\bbeta$? 
A simple example will help clarify this idea.  The equation $x^2=-1$ has no real solution. But $x\cdot x^2 = x\cdot (-1)$ has a real solution $\widehat{x} = 0$, in which case $\widehat{x}$ makes $x^2$ and $-1$ as close as possible.

\index{Change solution set}
\begin{example}[Multiplying from left can change the solution set]
Consider the matrix	
$$
\bX=\left[
\begin{matrix}
-3 & -4 \\
4 & 6  \\
1 & 1
\end{matrix}
\right] \mathrm{\,\,\,and\,\,\,}
\by=\left[
\begin{matrix}
1  \\
-1   \\
0
\end{matrix}
\right].
$$
It can be easily verified that $\bX\bbeta = \by$ has no solution for $\bbeta$ (an inconsistent system). However, if we multiply from left by 
$$
\bZ=
\begin{bmatrix}
0 & -1 & 6\\
0 & 1  & -4
\end{bmatrix}. 
$$
Then the new system $\bZ\bX\bbeta = \bZ\by$ has a solution: $\widehatbbeta = [1/2, -1/2]^\top$. 
This specific example shows why the normal equation can give rise to the least squares solution. Multiplying from the left of a linear system will change the solution set. The normal equation, especially, results in the least squares solution.
\end{example}

\section{OLS in Convex Optimization}\label{section:ls_conv}

\subsection*{Mathematical Tools}
We briefly introduce the concept of convex optimization.

\begin{definition}[Convex set]\label{definition:convexset}
A set $\sS\subseteq \real^n$ is called \textit{convex} if, for any $\bx,\by\in\sS$ and $\lambda\in[0,1]$, the point $\lambda\bx+(1-\lambda)\by$ also belongs to $\sS$.
\end{definition}
Geometrically, convex sets  contain all line segments
that join two points within the set (Figure~\ref{fig:cvxset}). Consequently, these sets do not feature any concave indentations.

\begin{figure}[h]
\centering       
\vspace{-0.25cm}                 
\subfigtopskip=2pt               
\subfigbottomskip=-2pt         
\subfigcapskip=-10pt      
\includegraphics[width=0.98\textwidth]{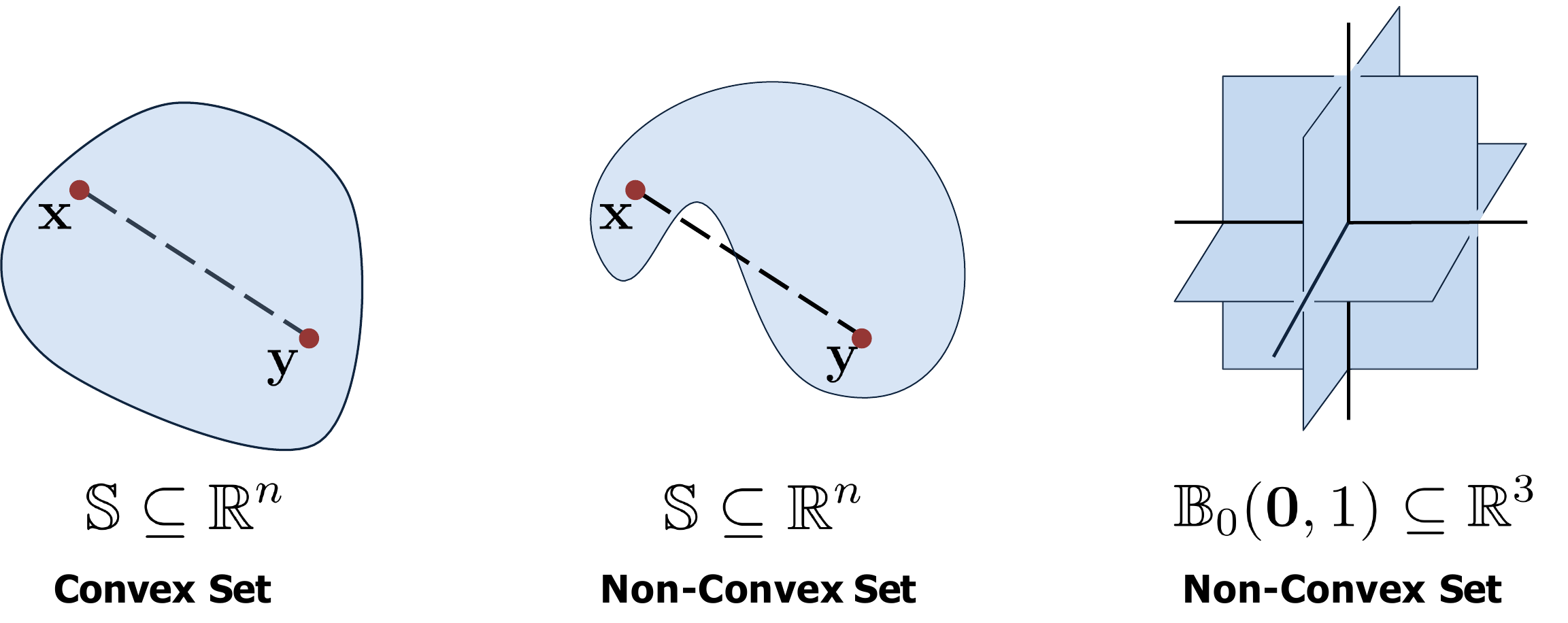}
\caption{A set is considered	 convex if it includes all convex combinations of its points. If there exists even one convex combination that lies outside the set, then by definition, the set is not convex. Therefore, a convex set must have a shape without any inward ``dents" or ``bulges". It's worth noting that the collection of sparse vectors does not satisfy this criterion and thus forms a non-convex set.}
\label{fig:cvxset}
\end{figure}

A related concept is that of convex functions, which exhibit specific behavior under convex combinations. We now recall the definition:
\begin{definition}[Convex functions]\label{definition:convexfuncs}
A function $f:\sS\rightarrow \real$ defined over a convex set $\sS\subseteq \real^n$ is called \textit{convex} if 
$$
f(\lambda\btheta+(1-\lambda)\bgamma)
\leq \lambda f(\btheta) + (1-\lambda) f(\bgamma), \text{ for any }\btheta,\bgamma\in\sS, \lambda\in[0,1].
$$
Moreover,   $f$ is called \textit{strictly convex} if 
$$
f(\lambda\btheta+(1-\lambda)\bgamma)
< \lambda f(\btheta) + (1-\lambda) f(\bgamma),\text{ for any }\btheta\neq \bgamma\in\sS, \lambda\in(0,1).
$$
\end{definition}
A well-known inequality derived from the concept of convex functions is provided below without a proof.
\begin{theoremHigh}[Jensen's inequality]\label{theorem:jensens_ineq}
Let $f: \sS \rightarrow \real$ be a convex function defined on a convex subset $\sS \subseteq \real^{n}$. For any finite sequence of points $\btheta_{1}, \btheta_{2}, \ldots, \btheta_{m} \in \sS$ and any sequence of nonnegative weights $\lambda_{1}, \lambda_{2},  \ldots, \lambda_{m}$ such that $\sum_{i=1}^{m} \lambda_{i} = 1$, Jensen's inequality states:
$$
f\left(\sum_{i=1}^{m} \lambda_{i} \btheta_{i}\right) \leq \sum_{i=1}^{m} \lambda_{i} f(\btheta_{i}).
$$
If $f$ is concave, the inequality is reversed.
%$$
%f\left(\sum_{i=1}^{m} \lambda_{i} \btheta_{i}\right) \geq \sum_{i=1}^{m} \lambda_{i} f(\btheta_{i}).
%$$
In the context of probability theory, if $\rvx$ is a random vector with values in $\sS$ and $f$ is a convex function, Jensen's inequality can be stated as follows:
$$
f(\Exp[\rvx]) \leq \Exp[f(\rvx)],
$$
where $\Exp[\cdot]$ denotes the expectation operator over the random vector $\rvx$. For a concave function, the inequality is again reversed.
\end{theoremHigh}

Convex functions do not necessarily have to be differentiable. However, when they are  differentiable, such functions can be characterized by the gradient inequality.
\begin{theoremHigh}[Gradient Inequality]\label{theorem:conv_gradient_ineq}
Let $f:\sS\rightarrow \real$ be a {continuously differentiable} function defined on a convex set $\sS\subseteq\real^n$. Then, $f$ is convex over $\sS$ if and only if 
\begin{equation}\label{equation:conv_gradient_ineq1}
f(\btheta) +\nabla f(\btheta)^\top (\bgamma-\btheta)\leq f(\bgamma), \text{ for any $\btheta,\bgamma\in\sS$}.
\end{equation}
Similarly, the function is strictly convex over $\sS$ if and only if 
\begin{equation}\label{equation:conv_gradient_ineq2}
f(\btheta) +\nabla f(\btheta)^\top (\bgamma-\btheta)< f(\bgamma), \text{ for any $\btheta\neq \bgamma\in\sS$}.
\end{equation}
This indicates that the graph of a convex function lies above its tangent plane at any point. For concave or strictly concave functions, the inequality signs are reversed.
\end{theoremHigh}

\begin{exercise}[Convexity of quadratic functions]\label{exercise:conv_quad}
Let $f(\btheta)=\frac{1}{2}\btheta^\top\bA\btheta+\bb^\top\btheta+c$, where $\bA\in\real^{n\times n}$ is symmetric, $\bb\in\real^n$, and $c\in\real$. Show that $f(\btheta)$ is convex (resp. strict convex) if and only if $\bA\succeq \bzero$ (resp. $\bA\succ\bzero$).
\end{exercise}

A \textit{convex optimization problem} (or simply a convex problem) involves minimizing a convex function over a convex set:
\begin{equation}\label{equation:convex_optim1} 
\textbf{(Convex Optimization)}:\qquad
\begin{aligned}
&\min \quad f(\btheta) \quad &&\text{(convex function)}\\
&\text{s.t.}\quad \btheta\in\sS \quad &&\text{(convex set)}.
\end{aligned}
\end{equation}

\begin{theoremHigh}[Local is global in convex optimization]\label{theorem:local_glob_conv}
Let $f: \sS \rightarrow \real$ be a convex function (resp. strictly convex function) defined over the convex set $\sS$. 
If $\widehatbtheta \in \sS$ is a local minimum of $f$ over $\sS$, then $\widehatbtheta$ is the global minimum (resp. strict global minimum, i.e., the only global minimum point) of $f$ over $\sS$.
\end{theoremHigh}
\begin{proof}[of Theorem~\ref{theorem:local_glob_conv}]
Since $\widehatbtheta$ is a local minimum of $f$ over $\sS$, there exists a scalar $\tau > 0$ such that $f(\btheta) \geq f(\widehatbtheta)$ for any $\btheta \in \sS$ satisfying $\btheta \in \sB[\widehatbtheta, \tau]$. Now let $\bgamma \in \sS$ satisfy $\bgamma \neq \widehatbtheta$. It suffices to show that $f(\bgamma) \geq f(\widehatbtheta)$. Let $\lambda \in (0, 1]$ be such that $\widehatbtheta + \lambda(\bgamma - \widehatbtheta) \in \sB[\widehatbtheta, \tau]$. An example of such $\lambda$ is $\lambda = \frac{\tau}{\normtwo{\widehatbtheta - \bgamma}}$. Since $\widehatbtheta + \lambda(\bgamma - \widehatbtheta) \in \sB[\widehatbtheta, \tau] \cap \sS$, it follows that
$
f(\widehatbtheta) \leq f(\widehatbtheta + \lambda(\bgamma - \widehatbtheta)),
$
and hence by Jensen's inequality (Theorem~\ref{theorem:jensens_ineq})
$
f(\widehatbtheta) \leq f(\widehatbtheta + \lambda(\bgamma - \widehatbtheta)) \leq (1 - \lambda) f(\widehatbtheta) + \lambda f(\bgamma).
$
Therefore, we obtain $f(\widehatbtheta) \leq f(\bgamma)$.

A slight modification of the above argument shows that any local minimum of a strictly convex function over a convex set is indeed  a strict global minimum of the function over the set.
\end{proof}

The optimal set of the convex problem \eqref{equation:convex_optim1}  is the set of all minimizers, that is, $\Theta=\argmin\{f(\btheta) : \btheta \in \sS\}$. This definition of an optimal set is also valid for general problems. 
A notable property of convex problems is that their optimal sets are also convex.

\begin{theoremHigh}[Convexity of the optimal set in convex optimization]\label{theorem:stric_op_str_conv}
Let $f: \sS \rightarrow \real$ be a convex function defined over the convex set $\sS \subseteq \real^n$. Then the set of optimal solutions of the problem,
$
\Theta=\argmin\{f(\btheta) : \btheta \in \sS\},
$ 
is convex. If, in addition, $f$ is strictly convex over $\sS$, then there exists \textbf{at most one} optimal solution.
\end{theoremHigh}
\begin{proof}[of Theorem~\ref{theorem:stric_op_str_conv}]
If $\Theta = \varnothing$, the result follows trivially. 
We then assume that $\Theta \neq \varnothing$ and denote the optimal value by $f^*$. Let $\btheta, \bgamma \in \Theta$ and $\lambda \in [0, 1]$. Then, by Jensen's inequality $f(\lambda \btheta + (1 - \lambda) \bgamma) \leq \lambda f^* + (1 - \lambda) f^* = f^*$, and hence $\lambda \btheta + (1 - \lambda) \bgamma$ is also optimal, i.e., belongs to $\Theta$, establishing the convexity of $\Theta$. Suppose now that $f$ is strictly convex and $\Theta$ is nonempty; to show that $\Theta$ is a singleton, suppose in contradiction that there exist $\btheta, \bgamma \in \Theta$ such that $\btheta \neq \bgamma$. Then $\frac{1}{2}\btheta + \frac{1}{2}\bgamma \in \sS$, and by the strict convexity of $f$ we have
$$
f\left(\frac{1}{2}\btheta + \frac{1}{2}\bgamma\right) < \frac{1}{2}f(\btheta) + \frac{1}{2}f(\bgamma) = \frac{1}{2}f^* + \frac{1}{2}f^* = f^*,
$$
which leads to a contradiction to the fact that $f^*$ is the optimal value.
\end{proof}

Stationarity is a \textbf{necessary optimality condition} for local optimality (Proposition~\ref{proposition:fermat_fist_opt}). However, when the objective function is additionally assumed to be convex, stationarity is a \textbf{necessary and sufficient condition} for optimality.

\begin{theoremHigh}[Necessity/sufficiency  of constrained convex]\label{thm:conv_stationary_optimality}
Let $ f:\sS\subseteq\real^n\rightarrow\real $ be a continuously differentiable convex function over a closed and convex set $ \sS$. Then $ \widehatbtheta $ is a stationary point  of 
$$
\text{(P)} \qquad \min_{\btheta \in \sS} f(\btheta)
$$
if and only if $ \widehatbtheta $ is an optimal solution of (P).
\end{theoremHigh}
\begin{proof}[of Theorem~\ref{thm:conv_stationary_optimality}]
If $ \widehatbtheta $ is an optimal solution of (P), then by Proposition~\ref{proposition:fermat_fist_opt}, it follows that $ \widehatbtheta $ is a stationary point of (P). To prove the sufficiency of the stationarity condition, assume that $ \widehatbtheta $ is a stationary point  of (P). For any $ \btheta \in \sS $, we have:
$$
f(\btheta) \geq f(\widehatbtheta) + \nabla f(\widehatbtheta)^\top (\btheta - \widehatbtheta) \geq f(\widehatbtheta),
$$
where the first inequality follows from the gradient inequality for convex functions (Theorem~\ref{theorem:conv_gradient_ineq}), and the second inequality follows from the definition of a stationary point. This shows that $ \widehatbtheta $ is indeed the global minimum point of (P), completing the proof.
\end{proof}

\subsection*{LS in Convex Optimization}

The theorems on convex functions help answer Questions Q2 and Q3 posed at the beginning of Section~\ref{section:ls_big_pic} rigorously.
We previously briefly  answer the  question Q2, the uniqueness of the least squares solution, in the large-sample least squares problem. If $n>p=\rank(\bX)$, the least squares solution $\widehatbbeta=(\bX^\top\bX)^{-1}\bX^\top\by$ is unique. 
In fact, this is the \textbf{only} case in which the least squares solution is guaranteed to be unique.

Note that both Q2 and Q3 can be formulated as convex optimization problems:
\begin{equation}
\begin{aligned}
\text{(P2)}: \qquad &\min_{\bbeta\in\real^p} f(\bbeta) = \normtwo{\by-\bX\bbeta}=\bbeta^\top\bX^\top\bX\bbeta - 2\by^\top\bX\bbeta + \by^\top\by;\\
\text{(P3)}:\qquad &\min_{\bbeta\in\sB} g(\bbeta) = \normtwo{\bbeta}^2, \quad \sB \triangleq \{\bbeta\in \real^p\mid \normtwo{\by-\bX\bbeta}^2 =\min \}.
\end{aligned}
\end{equation}
Here, both $f(\bbeta)$ and $g(\bbeta)$ are  convex functions (Exercise~\ref{exercise:conv_quad}), and the set $\sB$ is convex  (see Theorem~\ref{theorem:stric_op_str_conv}). 
Therefore, both (P2) and (P3) are convex optimization problems.
Theorem~\ref{theorem:stric_op_str_conv} proves that if the function is strictly convex, then the solution is unique.
Apparently, $g(\bbeta)$ is strictly convex (this again confirms Theorem~\ref{theorem:rank_def_ls_prop} and Q3). 
And $f(\bbeta)$ is strictly convex only when $\bX^\top\bX$ is positive definite, which is the case only when $\bX$ has full column rank (i.e., $n>p=\rank(\bX)$).
This answers the question Q2 rigorously.

\section{OLS in Fundamental Theorem of Linear Algebra}\label{section:ls-fundation-theorem}\label{section:ls_fund_la}
\subsection*{Fundamental Theorem of Linear Algebra}\index{OLS by algebra}

For any matrix $\bX\in \real^{n\times p}$, it can be easily  verified that any vector in the row space of $\bX$ is perpendicular to any vector in the null space of $\bX$. Suppose $\bbeta_n \in \nspace(\bX)$, then $\bX\bbeta_n = \bzero$ such that $\bbeta_n$ is perpendicular to every row of $\bX$, supporting our claim. This implies the row space of $\bX$ is the orthogonal complement to the null space of $\bX$.

Similarly, we can also show that any vector in the column space of $\bX$ is perpendicular to any vector in the null space of $\bX^\top$. Furthermore, the column space of $\bX$ together with the null space of $\bX^\top$ span the entire space of $\real^n$ which is known as the fundamental theorem of linear algebra.
\begin{theoremHigh}[Fundamental theorem of linear algebra]\label{theorem:fundamental-linear-algebra}
Orthogonal Complement and Rank-Nullity Theorem: for any matrix $\bX\in \real^{n\times p}$, we have 
\begin{itemize}
\item  The null space $\nspace(\bX)$ is the orthogonal complement to the row space $\cspace(\bX^\top)$ in $\real^p$: $\dim(\nspace(\bX))+\dim(\cspace(\bX^\top))=p$;

\item The null space $\nspace(\bX^\top)$ is the orthogonal complement to the column space $\cspace(\bX)$ in $\real^n$: $\dim(\nspace(\bX^\top))+\dim(\cspace(\bX))=n$;

\item For rank-$r$ matrix $\bX$, $\dim(\cspace(\bX^\top)) = \dim(\cspace(\bX)) = r$, that is, $\dim(\nspace(\bX)) = p-r$ and $\dim(\nspace(\bX^\top))=n-r$.
\end{itemize}
\end{theoremHigh}

\begin{figure}[h]
\centering
\includegraphics[width=0.98\textwidth]{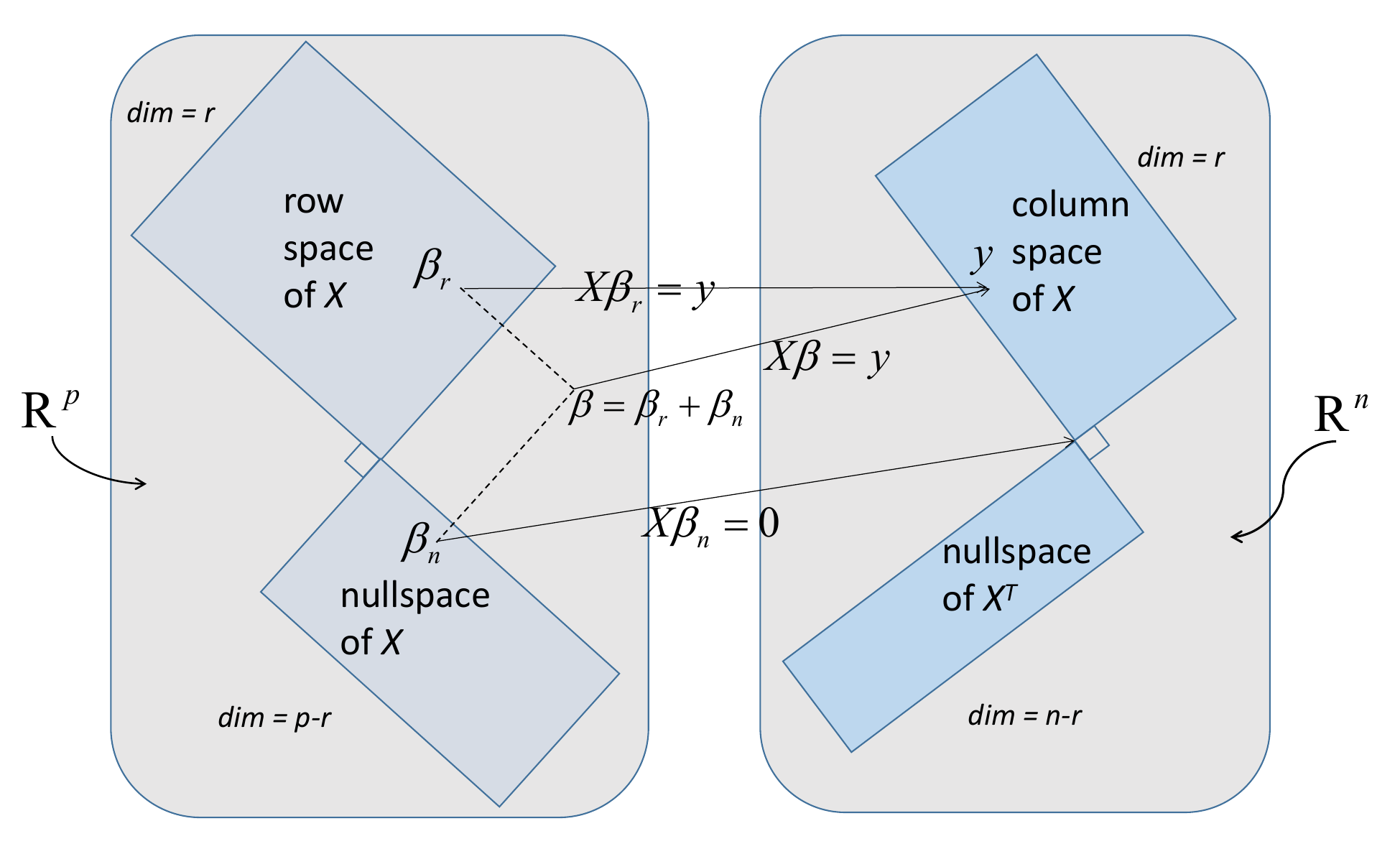}
\caption{Two pairs of orthogonal subspaces in $\real^p$ and $\real^n$. $\dim(\cspace(\bX^\top)) + \dim(\nspace(\bX)) = p$ and $\dim(\nspace(\bX^\top))+\dim(\cspace(\bX))=n$. The null space component goes to zero as $\bX\bbeta_{\bn} = \bzero \in \real^n$. 
Simultaneously, the row space component transforms into the column space by $\bX\bbeta_{\br} = \bX(\bbeta_{\br} + \bbeta_{\bn})=\by\in \cspace(\bX)$.}
\label{fig:lafundamental-ls}
\end{figure}
The fundamental theorem contains two parts, the dimension of the subspaces and the orthogonality of the subspaces. The orthogonality can be readily verified as we have shown at the beginning of this section. When the row space has dimension $r$, the null space has dimension $p-r$. This cannot be easily stated, and we prove it as follows.
\index{Fundamental theorem of linear algebra}
\begin{proof}[of Theorem~\ref{theorem:fundamental-linear-algebra}]
Following the proof of Lemma~\ref{lemma:equal-dimension-rank}, let $\br_1, \br_2, \ldots, \br_r$ be a set of vectors in $\real^p$ that form a basis for the row space. 
Then, \textcolor{black}{$\bX\br_1, \bX\br_2, \ldots, \bX\br_r$ is a basis for the column space of $\bX$}. Let $\bn_1, \bn_2, \ldots, \bn_k \in \real^p$ form a basis for the null space of $\bX$. Following again the proof of Lemma~\ref{lemma:equal-dimension-rank}, $\nspace(\bX) \bot \cspace(\bX^\top)$, thus, $\br_1, \br_2, \ldots, \br_r$ are perpendicular to $\bn_1, \bn_2, \ldots, \bn_k$. Then, $\{\br_1, \br_2, \ldots, \br_r, \bn_1, \bn_2, \ldots, \bn_k\}$ is linearly independent in $\real^p$.

For any vector $\bbeta\in \real^p $, $\bX\bbeta$ is in the column space of $\bX$. Thus, it can be expressed as a linear combination of $\bX\br_1, \bX\br_2, \ldots, \bX\br_r$: $\bX\bbeta = \sum_{i=1}^{r}a_i\bX\br_i$, which states that $\bX(\bbeta-\sum_{i=1}^{r}a_i\br_i) = \bzero$, and $\bbeta-\sum_{i=1}^{r}a_i\br_i$ is thus in $\nspace(\bX)$. Since $\{\bn_1, \bn_2, \ldots, \bn_k\}$ is a basis for the null space of $\bX$, $\bbeta-\sum_{i=1}^{r}a_i\br_i$ can be expressed by a combination of $\bn_1, \bn_2, \ldots, \bn_k$: $\bbeta-\sum_{i=1}^{r}a_i\br_i = \sum_{j=1}^{k}b_j \bn_j$, i.e., $\bbeta=\sum_{i=1}^{r}a_i\br_i + \sum_{j=1}^{k}b_j \bn_j$. That is, any vector $\bbeta\in \real^p$ can be expressed by $\{\br_1, \br_2, \ldots, \br_r, \bn_1, \bn_2, \ldots, \bn_k\}$ and the set forms a basis for $\real^p$. Thus the dimension add up to $p$: $r+k=p$, i.e., $\dim(\nspace(\bX))+\dim(\cspace(\bX^\top))=p$. Similarly, we can prove $\dim(\nspace(\bX^\top))+\dim(\cspace(\bX))=n$.
\end{proof}

Figure~\ref{fig:lafundamental-ls} demonstrates two pairs of such orthogonal subspaces and shows how $\bX$ takes $\bbeta$ into the column space. The dimensions of the row space of $\bX$ and the null space of $\bX$ add up to $p$. And the dimensions of the column space of $\bX$ and the null space of $\bX^\top$ add up to $n$. The null space component goes to zero as $\bX\bbeta_{\bn} = \bzero \in \real^n$, which is the intersection of the column space of $\bX$ and the null space of $\bX^\top$. The row space component transforms into the column space as $\bX\bbeta_{\br} = \bX(\bbeta_{\br} + \bbeta_{\bn})=\by\in \cspace(\bX)$.

\index{Fundamental theorem of linear algebra}
\subsection*{LS in Fundamental Theorem of Linear Algebra}

\begin{figure}[htp]
\centering
\includegraphics[width=0.98\textwidth]{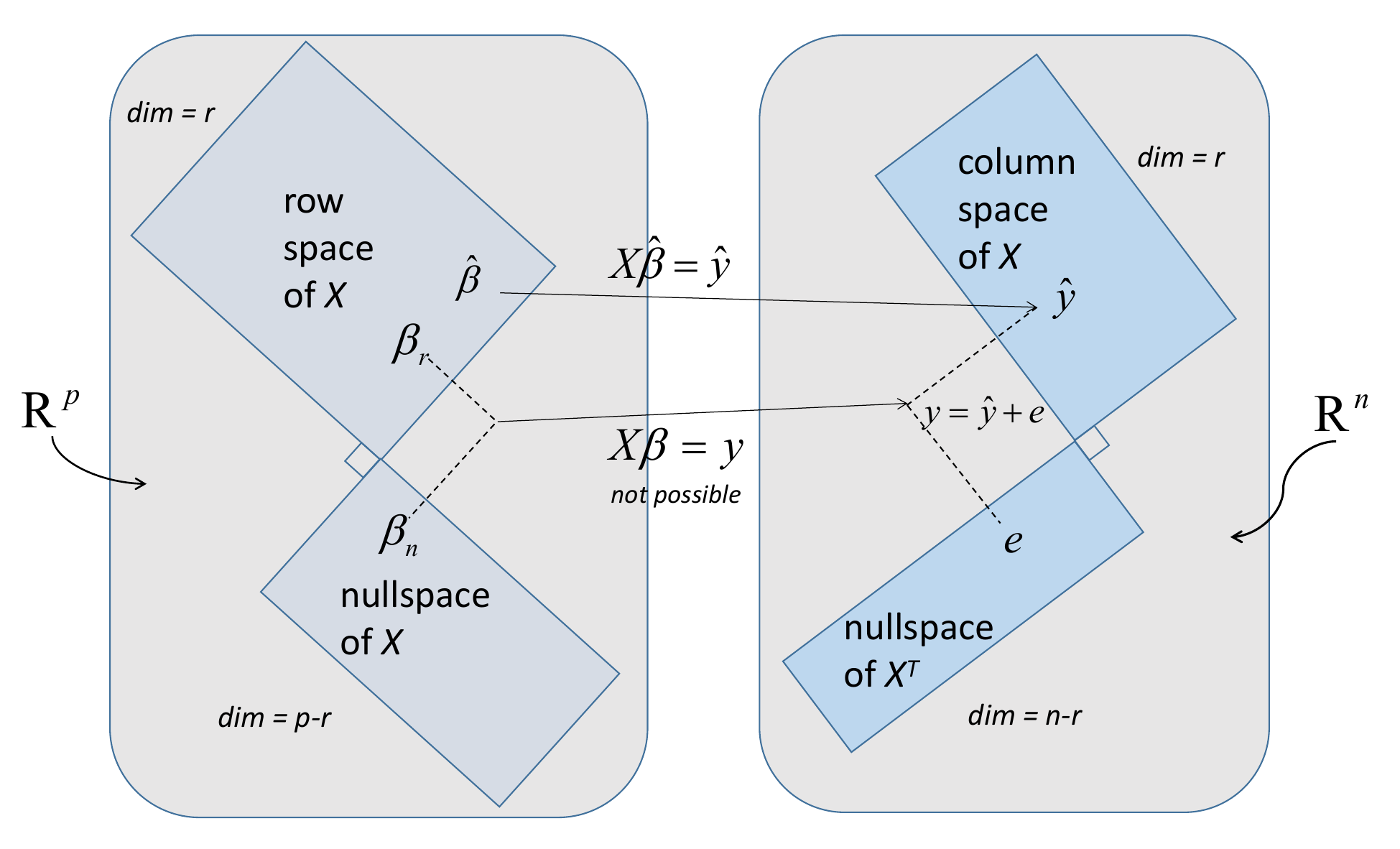}
\caption{Least squares: $\widehatbbeta$ minimizes $\normtwo{\by-\bX\bbeta}^2$. $\widehatbbeta$ is in the row space of $\bX$. $\be=\by-\bX\widehatbbeta$ is in the null space of $\bX^\top$.}
\label{fig:lafundamental2-LS}
\end{figure}

The solution to the least squares problem aims to minimize the error $\by-\bX\bbeta$ in terms of mean squared error.
Since $\bX\bbeta$ is a combination of the columns of $\bX$, it remains within the column space of $\bX$.  
Therefore,   the optimal choice is to select the nearest point to $\by$ within the column space \citep{strang1993fundamental, lu2021revisit}. This point is the projection $\widehat{\by}$ of $\by$ onto the column space of $\bX$. Then the error vector $\be=\by-\widehat{\by}$ has the minimum length. In other words, the best combination $\widehat{\by} = \bX\widehatbbeta$ is the projection of $\by$ onto the column space. The error $\be$ is perpendicular to the column space. Therefore, \textcolor{mydarkblue}{$\be=\by-\bX\widehatbbeta$ is in the null space of $\bX^\top$} (from the fundamental theorem of linear algebra):
$$
\bX^\top(\by-\bX\widehatbbeta) = \bzero  \qquad \text{or} \qquad \bX^\top\by=\bX^\top\bX\widehatbbeta,
$$
which agrees with the normal equation as we have defined in Section~\ref{section:ls_big_pic}. 
The relationship between $\be$ and $\widehat{\by}$ is shown in Figure~\ref{fig:lafundamental2-LS}, where $\by$ is decomposed into $\widehat{\by}+\be$. 
We can always find this decomposition since the column space of $\bX$ and the null space of $\bX^\top$ are orthogonal complement to each other, and they collectively span the entire space of $\real^n$.
%In addition, since the column space of $\bX$ and null space of $\bX^\top$ span the entire $\real^n$ space, any vector $\bv$ can be split into a vector $\widehat{\bv}$ in the column space of $\bX$ and a vector $\bv_{\be}$ in the null space of $\bX^\top$. 
Moreover, it can be demonstrated that the OLS estimate $\widehatbbeta$ resides in the row space of $\bX$, i.e., it cannot be decomposed into a combination of two components---one in the row space of $\bX$ and the other in the null space of $\bX$ (refer to the expression for  $\widehatbbeta$ via the pseudo-inverse of $\bX$ in Section~\ref{section:ls-via-svd}, where $\widehatbbeta$ is presented as a linear combination of the orthonormal basis of the row space; or refer to Theorem~\ref{theorem:unif_ls}).

To conclude, we avoid solving the equation $\by = \bX\bbeta$ by removing $\be$ from $\by$ and addressing $\widehat{\by} = \bX\widehatbbeta$ instead, i.e.,
\begin{equation}
\bX\bbeta=\by = \widehat{\by}+\be \,\, \mathrm{is\, impossible;} \qquad \bX\widehatbbeta=\widehat{\by}\,\, \mathrm{is\, possible.} \nonumber
\end{equation}

\index{Singular value decomposition}
\index{Full and reduced}
\section{OLS in  SVD for General Matrices}\label{section:ls-via-svd}

Prior to delving into the geometric aspects of least squares, we will first elucidate least squares through singular value decomposition (SVD), as they constitute fundamental components for the subsequent discussions.

\subsection{Least Squares via SVD for General Matrices}

Returning to the least squares problem, our prior assumption was that  $\bX$ has full rank. 
However, if $\bX$ does not have full column rank, $\bX^\top\bX$ becomes non-invertible. In such cases, we can employ the SVD decomposition of $\bX$ to address the least squares problem with a rank-deficient $\bX$.  
The methodology for solving the rank-deficient least squares problem is illustrated in the following theorem.

\index{Rank-deficient}
\index{Rank-deficiency}

\begin{theoremHigh}[LS via SVD for general matrices]\label{theorem:svd-deficient-rank}
Let $\bX\in \real^{n\times p}$, and let $\bX=\bU\bSigma\bV^\top$ be its full SVD decomposition with $\bU\in\real^{n\times n}$ and $\bV\in \real^{p\times p}$ being orthogonal matrices, and $\rank(\bX)=r \leq \min\{n,p\}$. Suppose $\bU=[\bu_1, \bu_2, \ldots, \bu_n]$ and  $\bV=[\bv_1, \bv_2, \ldots, \bv_p]$ are the column partitions of $\bU$ and $\bV$, respectively, and the observed output vector is $\by\in \real^n$. 
Then the ordinary least squares solution with the minimal $\ell_2$ norm to the linear system  $\bX\bbeta=\by$ is given by 
\begin{equation}\label{equation:svd-ls-solution}
\widehatbbeta = \sum_{i=1}^{r}\frac{\bu_i^\top \by}{\sigma_i}\bv_i = \bV\bSigma^+\bU^\top \by
\equiv \bX^+\by,
\end{equation}
where the upper-left side of $\bSigma^+ \in \real^{p\times n}$ is a diagonal matrix $\bSigma^+ = 
\footnotesize
\begin{bmatrix}
\bSigma_1^+ & \bzero \\
\bzero & \bzero
\end{bmatrix}$ with $\bSigma_1^+=\diag(\frac{1}{\sigma_1}, \frac{1}{\sigma_2}, \ldots, \frac{1}{\sigma_r})$, and $\bX^+$ denotes the pseudo-inverse of $\bX$ (Section~\ref{section:pseudo-in-svd}).
\end{theoremHigh}
%Note here $\bv_1, \bv_2, \ldots, \bv_r$ form a  basis for the row space of $\bX$ (Theorem~\ref{theorem:svd-four-orthonormal-Basis}). Therefore, $\widehatbbeta$ lies in the row space of $\bX$.

\begin{proof}[of Theorem~\ref{theorem:svd-deficient-rank}]
Expressing the loss to be minimized:
$$
\begin{aligned}
\normtwo{\by-\bX\bbeta}^2 
&= (\by-\bX\bbeta)^\top(\by-\bX\bbeta)
\stackrel{\dag}{=} (\by-\bX\bbeta)^\top\bU\bU^\top (\by-\bX\bbeta) \\
&\stackrel{\ddag}{=} \normtwo{\bU^\top \bX \bbeta-\bU^\top\by}^2
=\normtwo{\bU^\top \bX \bV\bV^\top \bbeta-\bU^\top\by}^2 \\
&\stackrel{*}{=} \normtwo{\bSigma\balpha - \bU^\top\by}^2  
=\sum_{i=1}^{r}(\sigma_i\alpha_i - \bu_i^\top\by)^2 +\sum_{i=r+1}^{n}(\bu_i^\top \by)^2,  
\end{aligned}
$$
where the equality $(\dag)$ follows since $\bU$ is an orthogonal matrix, and the equality ($\ddag$) follows from the invariance under orthogonal transformations, the equality ($*$) follows by letting  $\balpha\triangleq \bV^\top \bbeta$, and the last equality follows since $\sigma_{r+1}=\sigma_{r+2}= \ldots= \sigma_n=0$.

Since $\bbeta$ only appears in $\balpha$, setting  $\alpha_i = \frac{\bu_i^\top\by}{\sigma_i}$ for all $i\in \{1, 2, \ldots, r\}$ minimizes the  loss above. 
The result remains the same for any values of $\alpha_{r+1}, \alpha_{r+2}, \ldots, \alpha_{p}$. 
From a regularization point of view, we can set them to be 0 (the same as searching for minimum norm of $\bbeta$). 
This yields the SVD-based OLS solution:
$$
\widehatbbeta = \sum_{i=1}^{r}\frac{\bu_i^\top \by}{\sigma_i}\bv_i=\bV\bSigma^+\bU^\top \by = \bX^+\by,
$$
where $\bX^+ = \bV\bSigma^+\bU^\top\in \real^{p\times n}$ is known as the \textit{pseudo-inverse} of $\bX$.  Refer to Section~\ref{section:pseudo-inverse}  for a detailed discussion about the pseudo-inverse, where we also prove that \textit{the column space of $\bX^+$ is equal to the row space of $\bX$, and the row space of $\bX^+$ is equal to the column space of $\bX$}.
\end{proof}

\begin{figure}[h!]
\centering
\includegraphics[width=0.98\textwidth]{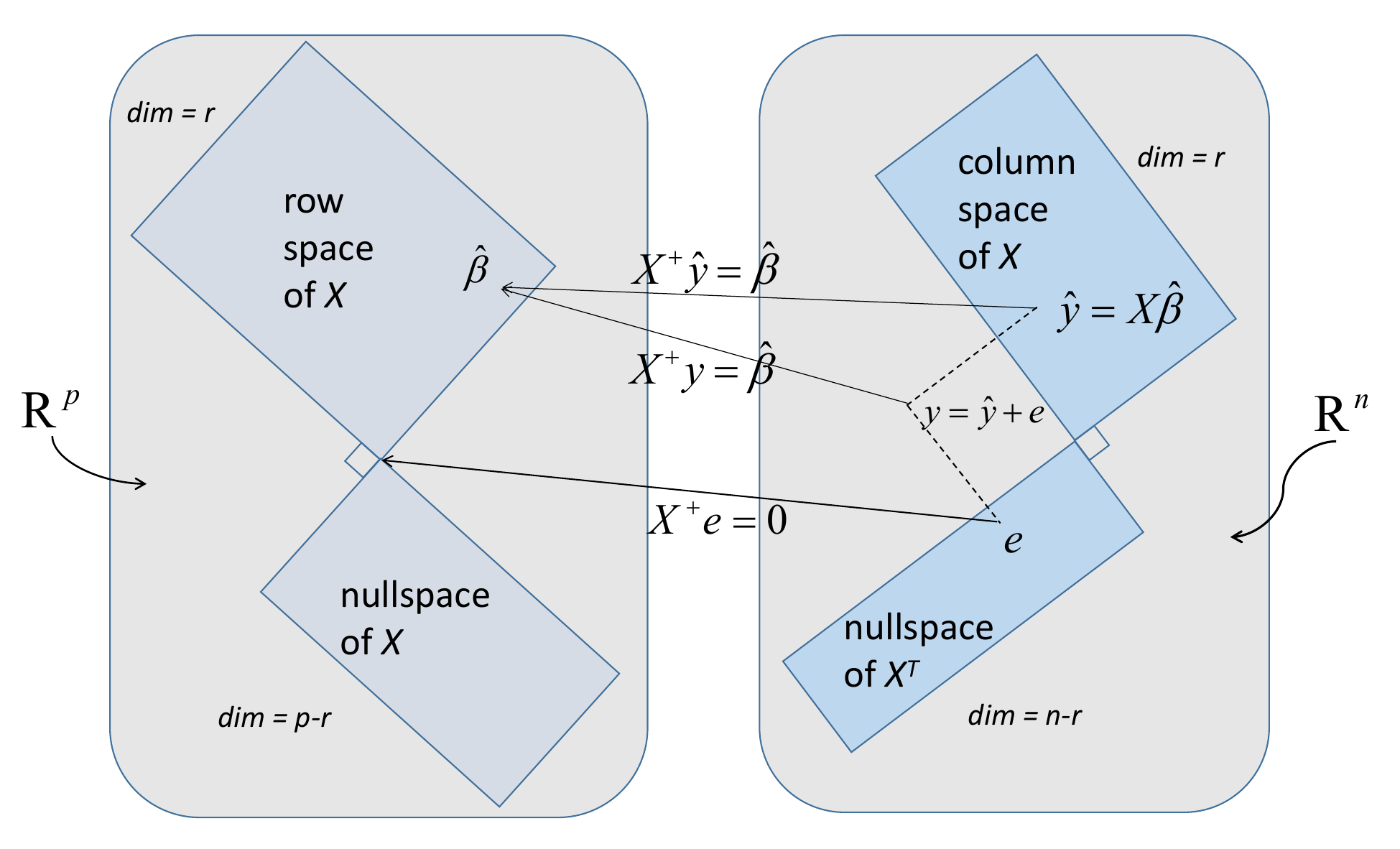}
\caption{$\bX^+$: Pseudo-inverse of $\bX$. A more detailed discussion of the four subspaces in pseudo-inverse is discussed in Section~\ref{section:subsec_pseudo_inv} (see  Figure~\ref{fig:lafundamental5-pseudo}).}
\label{fig:lafundamental4-LS-SVD}
\end{figure}
\index{Fundamental theorem of linear algebra}

\begin{proposition}[LS in the Four Subspaces of Linear Algebra via SVD]\label{proposition:fourspace-svd-ls}
Assume $\bX \in \real^{n\times p}$ is fixed and does \textbf{not} necessarily has full rank with $n\geq p$. Consider the overdetermined system $\by = \bX\bbeta$.
%the least squares solution by calculus via setting the derivative in every direction of $\normtwo{\by-\bX\bbeta}^2$ to be zero is $\widehatbbeta = (\bX^\top\bX)^{-1}\bX^\top\by$. The value $\widehatbbeta = (\bX^\top\bX)^{-1}\bX^\top\by$ is known as the ordinary least squares (OLS) estimate or simply least squares (LS) estimate of $\bbeta$.
Then, we can factor $\by$ into $\widehat{\by} + \be$, where $\widehat{\by}$ is in the column space of $\bX$ (in $\cspace(\bX)$), and $\be$ is in the null space of $\bX^\top$ (in $\nspace(\bX^\top)$). We can always find this decomposition since the column space of $\bX$ and the null space of $\bX^\top$ span the entire space $\real^n$. The relationship between vectors $\be$ and $\widehat{\by}$ is shown in Figure~\ref{fig:lafundamental4-LS-SVD}. Let $\bX^+ = \bV\bSigma^+\bU^\top$ be the pseudo-inverse of $\bX$. The pseudo-inverse $\bX^+$ agrees with $\bX^{-1}$ when $\bX$ is invertible. Then, we have the following properties (also shown in Figure~\ref{fig:lafundamental4-LS-SVD}):
\begin{itemize}
\item  For $\be\in \nspace(\bX^\top)$, it follows that $\bX^+\be = \bzero \in \real^p$.
\item  Given the OLS solution $\widehatbbeta$ via SVD, it follows that $\bX^+\widehat{\by} = \bX^+\by = \widehatbbeta$.
\item  OLS solution $\widehatbbeta$ is in the row space of $\bX$, i.e., it cannot be decomposed into a combination of two components that are in the row space of $\bX$ and the null space of $\bX$, respectively. This is the reason why $\widehatbbeta$, as shown  in Figure~\ref{fig:lafundamental4-LS-SVD}, is in the row space of $\bX$ rather than in $\real^p$ in general.
\end{itemize}
\end{proposition}
\begin{proof}[of Proposition~\ref{proposition:fourspace-svd-ls}]
Since $\be$ is in $\nspace(\bX^\top)$ and it is perpendicular to $\cspace(\bX)$, and we have shown in Theorem~\ref{theorem:svd-four-orthonormal-Basis} that $\{\bu_1,\bu_2, \ldots,\bu_r\}$ is an orthonormal basis of $\cspace(\bX)$, then the first $r$ components of $\bU^\top\be$ are all zeros. Therefore, $\bX^+\be = \bV\bSigma^+\bU^\top\be=\bzero$ (see also Figure~\ref{fig:lafundamental4-LS-SVD} where we transfer $\be$ from $\nspace(\bX^\top)$ into the zero vector $\bzero \in \real^p$ by $\bX^+$). Therefore, it follows that $\widehatbbeta=\bX^+\by = \bX^+(\widehat{\by}+\be) = \bX^+\widehat{\by}$.

Furthermore, we have also shown in Theorem~\ref{theorem:svd-four-orthonormal-Basis} that $\{\bv_1, \bv_2, \ldots, \bv_r\} $ is an orthonormal basis of $\cspace(\bX^\top)$. Thus, $\widehatbbeta = \sum_{i=1}^{r}\frac{\bu_i^\top \by}{\sigma_i}\bv_i$  is in the row space of $\bX$.
\end{proof}

In the following sections, we will also demonstrate that the vector $\widehat{\by}$ is the closest point to $\by$ within the column space of $\bX$. This point is the (orthogonal) projection $\widehat{\by}$ of $\by$ onto the column space of $\bX$. 
Then the error vector $\be=\by-\widehat{\by}$ has the minimum length (norm). 

Besides the OLS solution derived from SVD, practical implementations of solutions through normal equation may encounter numerical challenges when $\bX^\top\bX$ is close to singular. In particular, when two or more  columns in $\bX^\top\bX$ are nearly co-linear, 
the resulting parameter values can become excessively large.
Such near degeneracies will not be uncommon when dealing with real-world data sets. 
Addressing these numerical challenges can be effectively achieved through the application of SVD as well  \citep{bishop2006pattern}.

\index{Truncated SVD}
\index{Norm ratio}
\subsection{Least Squares with Norm Ratio Method}\label{section:svd_nmratio_method}

Continuing from the previous section, let $\bX_k \in \real^{n\times p}$ be the optimal rank-$k$ approximation to the original $n\times p$ matrix $\bX$ (Theorem~\ref{theorem:young-theorem_frob}). Define the \textit{Frobenius norm ratio} \citep{zhang2017matrix} as 
$$
\nu(k) \triangleq \frac{\normf{\bX_k}}{\normf{\bX}} = \frac{\sqrt{\sigma_1^2+\sigma_2^2+\ldots +\sigma_k^2}}{\sqrt{\sigma_1^2+\sigma_2^2+\ldots +\sigma_h^2}}, \quad h = \min\{n,p\},
$$
where  
$\bX_k$ is the truncated SVD of $\bX$ with the largest $k$ terms, i.e., $\bX_k = \sum_{i=1}^{k} \sigma_i\bu_i\bv_i^\top$ from the SVD of $\bX=\sum_{i=1}^{r} \sigma_i\bu_i\bv_i^\top$. And $\normf{\cdot}$ is the matrix Frobenius norm (Definition~\ref{definition:frobernius-in-svd}).
We determine the minimum integer $k$ satisfying 
$$
\nu(k) \geq \alpha
$$
as the \textit{effective rank estimate} $\widehat{r}$, where $\alpha$  is the threshold   capped at a maximum value of 1, and it is usually set to $\alpha=0.997$.
Once we have determined the effective rank $\widehat{r}$, we substitute it into Equation~\eqref{equation:svd-ls-solution}, yielding:
$$
\widehatbbeta = \sum_{i=1}^{\textcolor{mylightbluetext}{\widehat{r}}}\frac{\bu_i^\top \bb}{\sigma_i}\bv_i , 
$$
which can be regarded as an approximation to the OLS solution $\widehatbbeta$. 
And this solution corresponds to the OLS solution of the linear equation $\bX_{\widehat{r}}\bbeta = \bb$, where 
$$
\bX_{\widehat{r}} = \sum_{i=1}^{\widehat{r}} \sigma_i \bu_i\bv_i^\top.
$$
The introduced  filtering method  is particularly valuable  when dealing with a noisy matrix $\bX$ \citep{zhang2017matrix}.

\index{High-dimensional LS}
\index{Minimum-norm solution}
\subsection{High-Dimensional Least Squares Problems}

Although Theorem~\ref{theorem:svd-deficient-rank} applies to general types of matrices $\bX$,regardless of whether they have a large sample size or are rank-deficient,
we now focus on the problem of finding a solution to the linear system $\bX\bbeta = \by$
 where $\bX \in \real^{n \times p}$ has full row rank and $p > n$, i.e., the \textit{high-dimensional least squares problem}. Although there is no unique solution in general, the minimum ($\ell_2$) norm solution is unique. 
 The minimum-norm solution is defined as follows:
$$
\widehatbbeta_{\text{mn}} = \argmin_{\bbeta} \normtwo{\bbeta}^2\quad \text{s.t. }\bX\bbeta=\by.
$$
The solution to the minimum-norm problem is
$$
\widehatbbeta_{\text{mn}} = \bX^+\by= \bX^\top (\bX\bX^\top)^{-1} \by
$$
The above matrix $\bX^\top (\bX\bX^\top)^{-1}$ exists as long as $\bX$ has full row rank. To show that $\widehatbbeta_{\text{mn}}$ is a valid solution, we can substitute $\widehatbbeta_{\text{mn}}$ into the constraint equation:
$ \bX\widehatbbeta_{\text{mn}} = \bX\bX^\top (\bX\bX^\top)^{-1} \by = \by. $
To prove that $\widehatbbeta_{\text{mn}}$ has the smallest ($\ell_2$) norm among all solutions of $\bX\bbeta = \by$, we first show that the difference vector $(\bbeta - \widehatbbeta_{\text{mn}})$ is orthogonal to $\widehatbbeta_{\text{mn}}$. For any $\bbeta' \in \real^p$ such that $\bX\bbeta' = \by$, it follows that
$$
\begin{aligned}
(\widehatbbeta_{\text{mn}} - \bbeta')^\top \widehatbbeta_{\text{mn}} &= (\widehatbbeta_{\text{mn}} - \bbeta')^\top \bX^\top (\bX\bX^\top)^{-1} \by \\
&= (\bX(\widehatbbeta_{\text{mn}} - \bbeta'))^\top (\bX\bX^\top)^{-1} \by \\
&= (\by - \by)^\top (\bX\bX^\top)^{-1} \by
= \bzero,
\end{aligned}
$$
whence we have 
$$
\begin{aligned}
\normtwo{\bbeta'}^2 
&= \normtwobig{\bbeta' - \widehatbbeta_{\text{mn}} + \widehatbbeta_{\text{mn}}}^2 
= \normtwobig{\bbeta' - \widehatbbeta_{\text{mn}}}^2 + \normtwobig{\widehatbbeta_{\text{mn}}}^2
\geq \normtwobig{\widehatbbeta_{\text{mn}}}^2.
\end{aligned}
$$
This proves  that $\widehatbbeta_{\text{mn}} = \bX^\top (\bX\bX^\top)^{-1} \by$ is a solution to the minimum-norm problem.
Recall that any matrix $\bX$ can be written as the full SVD:
$ \bX = \bU \bSigma \bV^\top, $
where $\bU \in \real^{n \times n}$, $\bSigma \in \real^{n \times p}$, and $\bV \in \real^{p \times p}$. 
Since $\bX$ has full row rank. Substituting the SVD into the expression for $\widehatbbeta_{\text{mn}}$, we obtain:
$$
\begin{aligned}
\widehatbbeta_{\text{mn}} 
&= \bX^\top (\bX\bX^\top)^{-1} \by
= \bV \bSigma \bU^\top (\bU \bSigma \bV^\top \bV \bSigma \bU^\top)^{-1} \by \\
&= \bV \bSigma \bU^\top (\bU \bSigma^2 \bU^\top)^{-1} \by
= \bV \bSigma  ( \bSigma^2 )^{-1} \bU^\top \by
 = \bX^+\by,
\end{aligned}
$$
which again agrees with Theorem~\ref{theorem:svd-deficient-rank}.

\index{Orthogonal projection}
\index{Geometry interpretation}
\section{OLS in Geometry and Orthogonal Projection}\label{section:by-geometry-hat-matrix}
As discussed earlier, the OLS estimate involves minimizing $\normtwo{\by-\bX\bbeta}^2$, which searches for an estimate $\widehatbbeta$ such that $\bX\widehatbbeta$ is in $\cspace(\bX)$ so as to minimize the distance between $\bX\widehatbbeta$ and $\by$. The nearest point is the \textit{projection} $\widehat{\by}$. The predicted value $\widehat{\by} = \bX\widehatbbeta$ is the projection of $\by$ onto the column space  $\cspace(\bX)$ by a \textit{projection matrix} ${\bH \triangleq \bX(\bX^\top\bX)^{-1}\bX^\top}$ when $\bX$ has full column rank with $n\geq p$:
\begin{equation}
\widehat{\by} = \bX\widehatbbeta = \bH\by, \nonumber
\end{equation}
where the matrix $\bH$ is also known as the \textit{hat matrix}, since it ``put a hat" on $\by$ to produce $\widehatby$.
 
This shows $h_{ij}$, the entry $(i,j)$ of $\bH$ measures the influence or statistical leverage exerted on the prediction $\widehaty_i$ by the observation $y_j$.
Relatedly, if the $i$-th diagonal element of $\bH$ is particularly large, then the $i$-th data point ($i$-th row of $\bX$) is particularly sensitive or influential in determining the best LS fit, thus justifying the interpretation of the elements $h_{ii}$ as \textit{statistical leverage scores}~\footnote{These statistical leverage scores can be calculated using any semi-orthogonal matrix spanning the column space of $\bX$; see Problem~\ref{prob:statis_lev_semior}.} \citep{mahoney2011randomized}.
These leverage scores have been used extensively in classical regression diagnostics to identify potential outliers by, e.g., flagging data points with leverage score greater than 2 or 3 times the average value in order to be investigated as errors or potential outliers \citep{chatterjee1988sensitivity}.

But what is a projection matrix? Merely stating that $\bH = \bX(\bX^\top\bX)^{-1}\bX^\top$ is a projection requires elucidation. 
Before the discussion on the projection matrix, we first provide some basic properties about symmetric and idempotent matrices, which will find extensive application in subsequent sections.

\subsection{Properties of Symmetric and Idempotent Matrices}\index{Idempotent matrix}

Symmetric idempotent matrices exhibit specific eigenvalues, a crucial aspect for the subsequent sections on the distribution theory of least squares.

\begin{lemma}[Eigenvalue of symmetric idempotent matrices]\label{proposition:eigenvalues-of-projection}
The only possible eigenvalues of any symmetric idempotent matrix are 0 and 1.
\end{lemma}
In Lemma~\ref{proposition:eigenvalues-of-projection2}, we will show that the eigenvalues of idempotent matrices (not necessarily symmetric) are 1 and 0 as well, which relaxes the conditions required here (both idempotent and symmetric). However, the method used in the proof is quite useful so we keep both of the claims.
\begin{proof}[of Lemma~\ref{proposition:eigenvalues-of-projection}]
Let $\bX$ be a symmetric idempotent matrix. By spectral theorem (Theorem~\ref{theorem:spectral_theorem}), we can decompose $\bX = \bQ \bLambda\bQ^\top$, where $\bQ$ is an orthogonal matrix, and $\bLambda$ is a diagonal matrix. Therefore, it follows that 
\begin{equation}
\begin{aligned}
(\bQ\bLambda\bQ^\top)^2 
&= \bQ\bLambda\bQ^\top 
\qquad\implies \qquad
\bQ\bLambda\bQ^\top\bQ\bLambda\bQ^\top = \bQ\bLambda\bQ^\top\\
\qquad\implies\qquad
\bQ\bLambda^2\bQ^\top &= \bQ\bLambda\bQ^\top 
\qquad\implies \qquad
\bLambda^2 = \bLambda
\qquad\implies \qquad
\lambda_i^2 = \lambda_i,
\end{aligned}
\end{equation}
where the first equality follows since $\bX$ is symmetric and idempotent.
Thus, the eigenvalues of $\bX$ satisfy that $\lambda_i \in \{0,1\},\ \forall\, i$. This completes the proof.
\end{proof}

In the previous lemma, we used the spectral theorem to show that the eigenvalues of any symmetric idempotent matrix are 0 or 1. This approach is common in linear algebra and appears frequently in statistical theory (see later sections on distribution theory; Chapter~\ref{sec:lr-gaussian-noise}). With a slight modification, we can remove the symmetry condition entirely and extend the result to general idempotent matrices.
\begin{lemma}[Eigenvalue of idempotent matrices]\label{proposition:eigenvalues-of-projection2}
The only possible eigenvalues of any idempotent matrix are 0 and 1.
\end{lemma}
\begin{proof}[of Lemma~\ref{proposition:eigenvalues-of-projection2}]
Let $\bbeta$ denote an eigenvector of the idempotent matrix $\bX$ corresponding to the eigenvalue $\lambda$. That is,
$\bX \bbeta = \lambda \bbeta$.
Also, we have 
$$
\begin{aligned}
\bX^2\bbeta 
&=(\bX^2)\bbeta =\bX\bbeta = \lambda\bbeta
=\bX(\bX\bbeta) = \bX(\lambda\bbeta)=\lambda\bX\bbeta=\lambda^2\bbeta,
\end{aligned}
$$
which implies $\lambda^2 = \lambda$, and $\lambda$ is either 0 or 1. This completes the proof.
\end{proof}

We also demonstrate that the rank of a symmetric idempotent matrix is equal to its trace, a result that will be highly beneficial in the subsequent sections.
\begin{lemma}[Rank and trace of symmetric idempotent matrices]\label{lemma:rank-of-symmetric-idempotent}
For any $n\times n$ symmetric idempotent matrix $\bX$, the rank of $\bX$ equals its trace.
\end{lemma}
\begin{proof}[of Lemma~\ref{lemma:rank-of-symmetric-idempotent}]
From Spectral Theorem~\ref{theorem:spectral_theorem}, the matrix $\bX$ admits the spectral decomposition $\bX = \bQ\bLambda\bQ^\top$. 
Since $\bX$ and $\bLambda$ are similar matrices, their rank and trace are the same (see Lemma~\ref{lemma:eigenvalue-similar-matrices}). That is, 
$$
\begin{aligned}
\rank(\bX) &= \rank(\diag(\lambda_1, \lambda_2, \ldots, \lambda_n));\\
\trace(\bX) &= \trace(\diag(\lambda_1, \lambda_2, \ldots, \lambda_n)),\\
\end{aligned}
$$
By Lemma~\ref{proposition:eigenvalues-of-projection}, the only eigenvalues of $\bX$ are 0 and 1. Then, it follows that 
$\rank(\bX) = \trace(\bX)$.
\end{proof}

\index{Reduced row echelon form}
\index{Idempotent matrix}
In the  previous lemma, we prove the rank and trace of any symmetric idempotent matrix are the same. 
However, this result also holds under a weaker condition---namely, just idempotency. We now present a more general version of the lemma. Although the second proof applies to a broader class of matrices, we again include both versions because the techniques used in each are valuable and commonly applied in linear algebra.
\begin{lemma}[Rank and trace of an idempotent matrix]\label{lemma:rank-of-symmetric-idempotent2}
For any $n\times n$ idempotent matrix $\bX$, the rank of $\bX$ equals  its trace.
\end{lemma}
\begin{proof}[of Lemma~\ref{lemma:rank-of-symmetric-idempotent2}]
Any $n\times n$ rank-$r$ matrix $\bX$ admits CR decomposition $\bX = \bC\bR$, where $\bC\in\real^{n\times r}$ and $\bR\in \real^{r\times n}$ have full rank $r$ 
(see Section~\ref{section:cr-decomposition}).
Then, it follows that 
$$
\begin{aligned}
\bX^2 &= \bX 
\quad\implies \quad
\bC\bR\bC\bR = \bC\bR 
\quad\implies \quad
\bR\bC\bR =\bR 
\quad\implies \quad
\bR\bC =\bI_r,
\end{aligned}
$$ 
where $\bI_r$ is the $r\times r$ identity matrix. Thus, the trace is
$$
\trace(\bX) = \trace(\bC\bR) =\trace(\bR\bC) =\trace(\bI_r) = r, 
$$
where the second equality uses the cyclic property of the trace, which completes the proof.
\end{proof}

\index{Projection matrix}
\subsection{By Geometry and Orthogonal Projection}\label{section:ortho_geom_ls}

Formally, we define the {projection matrix} as follows:
\begin{definition}[Projection matrix]\label{definition:projection-matrix}
A matrix $\bH\in \real^{n\times n}$ is called a \textit{projection matrix} or \textit{projector} onto a subspace $\mathcalV \in \real^n$ if and only if $\bH$ satisfies the following properties:
\begin{itemize}
\item (P1). $\bH\by \in \mathcalV$ for all $\by \in \real^n$: Any vector can be projected onto the subspace $\mathcalV$.
\item (P2). $\bH\by =\by$ for all $\by \in \mathcalV$: Projecting a vector that is already in that subspace has no further effect.
\item (P3). $\bH^2 = \bH$, i.e., applying the projection twice is the same as applying it once, because the vector is already in the subspace. This property is known as idempotence.
\end{itemize}
\end{definition}

Since we project a vector in $\real^n$ onto a subspace of $\real^n$, any projection matrix must be square. 
Otherwise, we will project onto the subspace of $\real^m$ rather than $\real^n$.
We realize that $\bH\by$ is always in the column space of $\bH$. 
One might then ask: what is the relationship between the subspace $\mathcalV$ and the column space $\cspace(\bH)$. 
In fact, the column space of $\bH$ is equal to the subspace $\mathcalV$ onto which we are projecting. 

Suppose $\mathcalV=\cspace(\bH)$, and suppose further that $\by$ is already in the subspace $\mathcalV=\cspace(\bH)$, i.e., there is a vector $\balpha$ such that $\by=\bH\balpha$. Given  only the condition (P3) above, we have,
$$
\bH\by = \bH\bH\balpha = \bH\balpha = \by.
$$
That is, condition (P3) implies conditions (P1) and (P2). 
Therefore, the definition of a projection matrix can be simplified to require only that $\bH$ is idempotent.

Intuitively, we also want the projection $\widehat{\by}=\bH\by$ of any vector $\by$ to be perpendicular to the residual vector $\by - \widehat{\by}$ such that the distance between $\widehat{\by}$ and $\by$ is minimized, which aligns with the principle of least squares error minimization. Such a projection is called an \textit{orthogonal projection}.

\index{Orthogonal projection}
\begin{definition}[Orthogonal and oblique projection matrix]\label{definition:orthogonal-projection-matrix}
A matrix $\bH$ is called an \textit{orthogonal projection matrix} or an \textit{orthogonal projector}  onto a subspace $\mathcalV \in \real^n$ if and only if $\bH$ is a projection matrix, and the projection $\widehat{\by}$ of any vector $\by\in \real^n$ is orthogonal to $\by - \widehat{\by}$, i.e., $\bH$ projects onto $\mathcalV$ and along $\mathcalV^\perp$, the orthogonal complement of $\mathcalV$.

Otherwise, if $\widehat{\by}$ is not orthogonal to $\by - \widehat{\by}$, then the projection matrix is called an \textit{oblique projection matrix} or an \textit{oblique projector}. A comparison between orthogonal and oblique projections is shown in Figure~\ref{fig:ls-geometric1-compare}.
\end{definition}

\begin{figure}[h!]
\centering  
\vspace{-0.35cm} 
\subfigtopskip=2pt  
\subfigbottomskip=2pt  
\subfigcapskip=-5pt  
\subfigure[Orthogonal projection: project $\by$ to $\widehat{\by}$.]{\label{fig:ls-geometric1}
\includegraphics[width=0.47\linewidth]{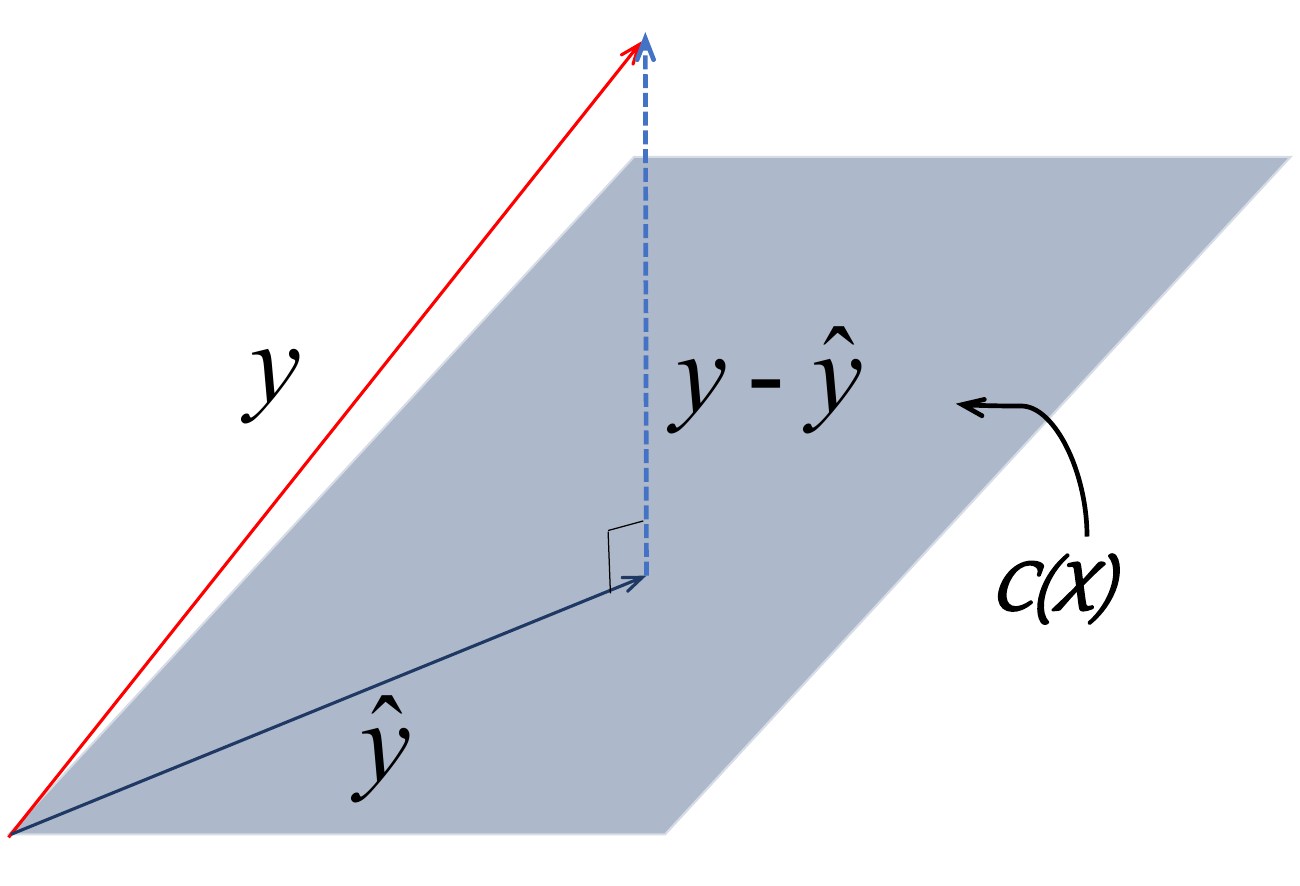}}
\quad 
\subfigure[Oblique projection: project $\by$ to $\widehat{\by}_1$ or $\widehat{\by}_2$.]{\label{fig:ls-geometric1-oblique}
\includegraphics[width=0.47\linewidth]{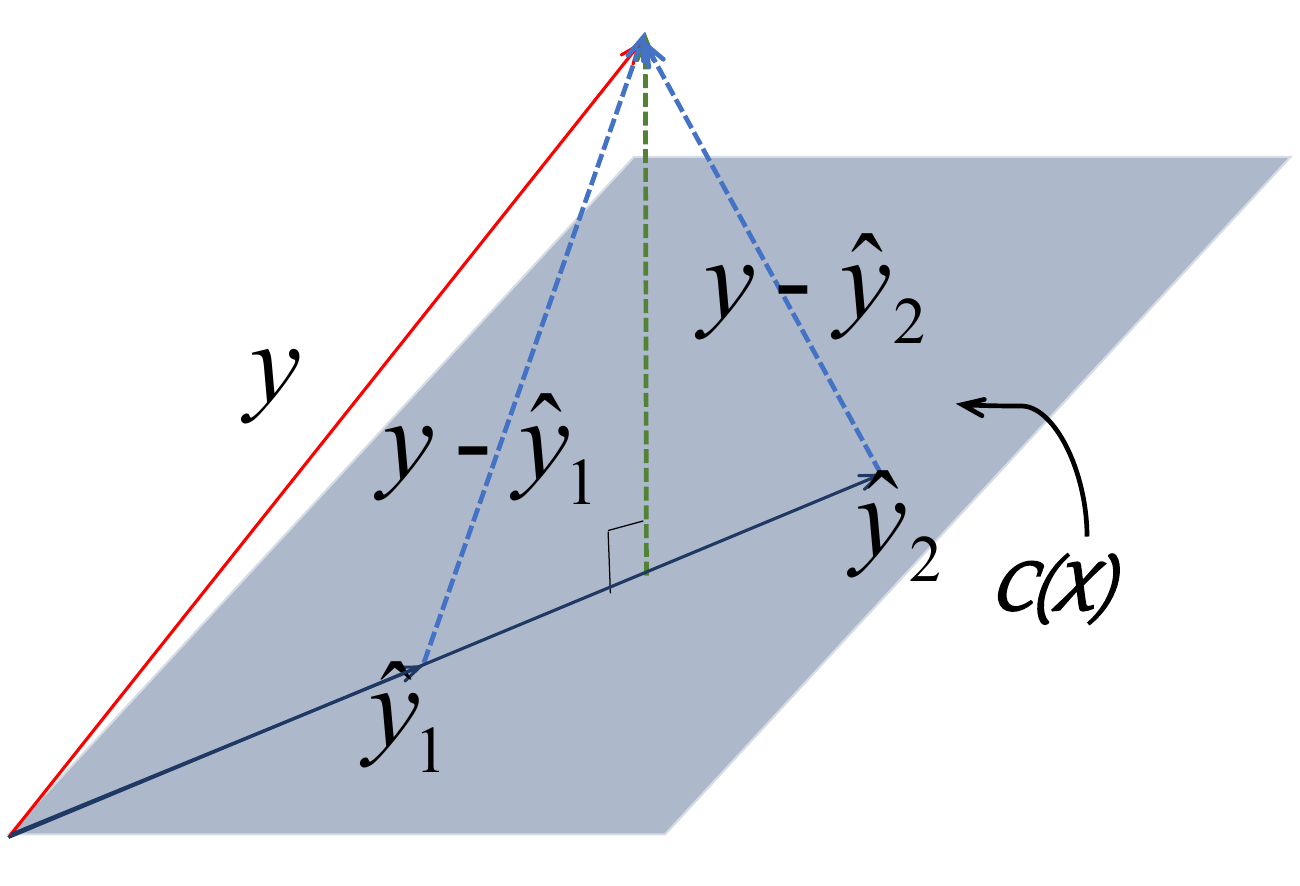}}
\caption{Projection onto the hyperplane of $\cspace(\bX)$, i.e., the column space of $\bX$.}
\label{fig:ls-geometric1-compare}
\end{figure}

Note that in the context of orthogonal projection, the term does not imply that the projection matrix itself is orthogonal (Definition~\ref{definition:orthogn_mat}). 
Instead, it means that the projected vector $\widehat{\by}$ is perpendicular to the residual vector $\by - \widehat{\by}$. This specialized orthogonal projection matrix will be implicitly assumed as such in the subsequent discussion unless explicitly clarified.

An \textit{elementary projector} is a projector exactly one of whose eigenvalues is 0.
Then we have the following result:
\begin{exercise}[Elementary projection matrix]\label{exercise:projection_matrix_intro2}
Let $\bx\in\real^n$ be nonzero. Show that $\bH\triangleq \bI-\frac{1}{\bx^\top\bx}\bx\bx^\top$ is an elementary projection matrix satisfying (a). $\rank(\bH)=n-1$;  (b). $\cspace(\bH)=\spn\{\bx\}^\perp$; (c). $\nspace(\bH)=\spn\{\bx\}$.
For the other way around, if $\bH\in\real^{n\times n}$ is a projector with $\rank(\bH)=n-1$, show that there is a nonzero vector $\bx\in\nspace(\bH)$ such that $\bH= \bI-\frac{1}{\bx^\top\bx}\bx\bx^\top$.
\end{exercise}

\begin{lemma}[Symmetric orthogonal projection matrix]\label{lemma:symmetric-projection-matrix}
A projection matrix $\bH$ is an orthogonal projection matrix if and only if $\bH$ is symmetric.
That is:
\begin{itemize}
\item If $\bH^2=\bH$ and $\bH^\top=\bH$, then $\bH$ is an orthogonal projector.
\item If $\bH^2=\bH$ and $\bH^\top\neq \bH$, then $\bH$ is an oblique projector.
\end{itemize}
\end{lemma}
\begin{proof}[of Lemma~\ref{lemma:symmetric-projection-matrix}]
Suppose $\bH$ is an orthogonal projection matrix, which projects vectors onto a subspace $\mathcalV$. Then any vectors $\bv$ and $\bw$ can be decomposed into a vector lies in $\mathcalV$ ($\bv_p$ and $\bw_p$) and a vector lies in $\mathcalV^\perp$ ($\bv_n$ and $\bw_n$), so that
$$
\begin{aligned}
\bv &= \bv_p + \bv_n
\qquad \text{and}\qquad 
\bw = \bw_p + \bw_n.
\end{aligned}
$$
Since the projection matrix $\bH$  projects vectors onto $\mathcalV$, it follows that $\bH\bv = \bv_p$ and $\bH\bw = \bw_p$, whence we have 
$$
\begin{aligned}
(\bH\bv)^\top \bw
&=\bv_p^\top\bw=\bv_p^\top (\bw_p + \bw_n)   &\qquad  \bv^\top (\bH\bw) &=\bv^\top \bw_p =(\bv_p+\bv_n)^\top \bw_p \\
&=\bv_p^\top \bw_p + \bv_p^\top\bw_n=\bv_p^\top \bw_p; &\qquad &=\bv_p^\top \bw_p  +\bv_n^\top \bw_p=\bv_p^\top \bw_p,
\end{aligned}
$$
where the last equations follow from the fact that $\bv_p$ is perpendicular to $\bw_n$, and $\bv_n$ is perpendicular to $\bw_p$. 
Therefore, we conclude that
$$
\begin{aligned}
(\bH\bv)^\top \bw = \bv^\top (\bH\bw) 
\quad\implies\quad 
\bv^\top \bH^\top \bw = \bv^\top \bH\bw,
\end{aligned}
$$
which implies $\bH^\top = \bH$.

Conversely, if a projection matrix $\bH$ (not necessarily an orthogonal projection) is symmetric, then any vector $\bv$ can be decomposed into $\bv = \bH\bv + (\bI-\bH)\bv$. If we can prove $\bH\bv$ is perpendicular to $(\bI-\bH)\bv$, then we complete the proof. To see this, we have 
$$
\begin{aligned}
(\bH\bv)^\top (\bI-\bH)\bv 
&= \bv^\top \bH^\top (\bI-\bH)\bv
=\bv^\top  (\bH^\top-\bH^\top\bH)\bv \\
&=\bv^\top  (\bH-\bH\bH)\bv
=\bv^\top  (\bH-\bH)\bv = \bzero,
\end{aligned}
$$
which completes the proof.
\end{proof}

We claimed earlier that orthogonal projection minimizes the distance between a vector $\by $ and its projection $\widehat{\by}$. 
We now rigorously prove this important property.
\begin{theoremHigh}[Minimum distance in orthogonal projection]\label{theorem:minimal-distance-orthogonal-projection}
Let $\mathcalV$ be a subspace of $\real^n$ and $\bH\in\real^{n\times n}$ be an orthogonal projection matrix onto $\mathcalV$. Then, given any vector $\by\in\real^n$, it follows that
$$
\normtwo{\by - \bH\by}^2 \leq \normtwo{\by - \bv}^2, \qquad \forall\, \bv \in \mathcalV.
$$
\end{theoremHigh}
\index{Spectral decomposition}
\begin{proof}[of Theorem~\ref{theorem:minimal-distance-orthogonal-projection}]
Let $\bH = \bQ\bLambda\bQ^\top \in \real^{n\times n}$ be the spectral decomposition of the orthogonal projection matrix $\bH$, 
where $\bQ=[\bq_1, \bq_2, \ldots, \bq_n]$ is the column partition of $\bQ$, and $\bLambda=\diag(\lambda_1, \lambda_2, \ldots, \lambda_n)$. Let $\dim(\mathcalV) = r$. Then, from Lemma~\ref{proposition:eigenvalues-of-projection}, the only possible eigenvalues of the orthogonal projection matrix are 1 and 0. Without loss of generality, let $\lambda_1=\lambda_2=\ldots=\lambda_r=1$ and $\lambda_{r+1}=\lambda_{r+2}=\ldots=\lambda_n=0$. Then, it follows that 
\begin{itemize}
\item $\{\bq_1, \bq_2, \ldots, \bq_n\}$ is an orthonormal basis of $\real^n$.
\item $\{\bq_1, \bq_2, \ldots, \bq_r\}$ is an orthonormal basis of $\mathcalV$. So for any vector $\bv\in \mathcalV$, we have $\bv^\top\bq_i=0$ for $i\in \{r+1, r+2, \ldots, n\}$.
\end{itemize}
 Then we have,
$$
\small
\begin{aligned}
\normtwo{\by - \bH\by}^2 
&\stackrel{\dag}{=} \normtwo{\bQ^\top\by - \bQ^\top\bH\by}^2 
= \sum_{i=1}^{n}(\by^\top \bq_i - (\bH\by)^\top \bq_i)^2
\stackrel{\ddag}{=}\sum_{i=1}^{n}(\by^\top \bq_i - \by^\top\bH \bq_i)^2 \\
&\stackrel{*}{=}\sum_{i=1}^{n}(\by^\top \bq_i - \lambda_i\by^\top \bq_i)^2
\stackrel{+}{=}0+\sum_{i=r+1}^{n}(\by^\top \bq_i)^2 
\leq \sum_{i=1}^{r}(\by^\top\bq_i - \bv^\top\bq_i)^2+ \sum_{i=r+1}^{n}(\by^\top \bq_i)^2 \\
&\stackrel{\bot}{=}\normtwo{\bQ^\top \by - \bQ^\top \bv}^2  
=\normtwo{\by - \bv}^2,
\end{aligned}
$$
where the equality ($\dag$) follows from the invariance under orthogonal transformation, the equality ($\ddag$) follows since $\bH$ is symmetric, 
the equality ($*$) follows from $\bH\bQ =  \bQ\bLambda$, the equality ($+$) follows since the eigenvalues are 1 or 0, and the equality ($\bot$) follows since $\bv^\top\bq_i=0$ for $i>r$. 
This completes the proof.
\end{proof}

Next, we examine the geometric relationship between a vector $\by$ and its orthogonal projection $\bH\by$.
\begin{lemma}[Angle between the original and projected vectors]\label{lemma:angle-orthogonal-projection}
Let $\bH$ be an orthogonal projection onto $\mathcalV$. Then,
\begin{enumerate}[(i)]
\item  $\by^\top (\bH\by) \geq 0$, meaning the angle between $\by$ and $\bH\by$ is less than or equal to $90^\circ$;
\item  $\normtwo{\bH\by}^2 \leq \normtwo{\by}^2$, meaning the length of the projected vector is no greater than the original vector.
\end{enumerate}
\end{lemma}
\begin{proof}[of Lemma~\ref{lemma:angle-orthogonal-projection}]
According to the definition of the orthogonal projection, we have $\by^\top (\bH\by) =\by^\top \bH (\bH\by)=\by^\top \bH^\top (\bH\by) =\normtwo{\bH\by}^2 \geq 0$.
And we could decompose $\by$ by 
$$
\begin{aligned}
\normtwo{\by}^2 = \normtwo{(\bI-\bH+\bH)\by}^2 
&= \normtwo{(\bI-\bH)\by}^2 + \normtwo{\bH\by}^2 + 2 \by^\top (\bI-\bH)^\top\bH\by\\
&= \normtwo{(\bI-\bH)\by}^2 + \normtwo{\bH\by}^2 \geq \normtwo{\bH\by}^2.
\end{aligned}
$$
This completes the proof.
\end{proof}

%Moving forward, we establish additional connections between the OLS solution and the orthogonal projection.
%\textbf{To conclude}, on the origin of projection and orthogonal projection matrices, we define the projection matrix to be idempotent, and it is symmetric when restricted to orthogonal projection. From this orthogonal projection, we prove the distance between the original vector and the projected vector is minimal.
%\paragraph{Projection in matrix decomposition.} Following the orthogonal projection discussed in the  proposition above, $\bH = \bX(\bX^\top\bX)^{-1}\bX^\top$. Let $\bX = \bQ\bR$ and $\bX=\bU\bSigma\bV^\top$ be the QR and SVD decomposition of $\bX$, respectively. Then the orthogonal projection can also be equivalently denoted as $\bH = \bQ\bQ^\top = \bU\bU^\top$.
%The hat matrix $\bH$ has two important properties. First, it is symmetric: $\bH = \bH^\top$. Second, it is idempotent : $\bH=\bH^2$. The two properties here can be easily verified.
%Back to the calculus view of least squares. Choose $\widehatbbeta$ to minimize $ \normtwo{\by-\bX\widehatbbeta}^2$ is equivalent to choose a vector $\boldeta$ in $\cspace(\bX)$ that minimize $\normtwo{\by-\boldeta}^2$. From Proof of Theorem~\ref{theorem:ols}, $\boldeta$ is actually equal to $\bH\by=\bX(\bX^\top\bX)^{-1}\bX^\top\by$.

In conclusion, to determine the OLS solution, we define the \textit{projection matrix} as an idempotent matrix. 
For it to represent an orthogonal projection, we add the condition that the matrix must also be symmetric. Through this orthogonal projection, we illustrate that the distance between the original vector and its projection is minimized.

We now further explore the relationship between the OLS solution and orthogonal projection.
\begin{proposition}[Projection matrix from a set of vectors]\label{proposition:projection-from-matrix}
Let $\bx_1, \bx_2, \ldots, \bx_p\in\real^n$ be linearly independent vectors such that $\cspace([\bx_1, \bx_2, \ldots, \bx_p]) = \mathcalV$, and assume $n\geq p$. 
Then, the orthogonal projection onto the subspace $\mathcalV$ can be expressed  as:
$$
\bH = \bX(\bX^\top\bX)^{-1}\bX^\top,
$$
where $\bX\in \real^{n\times p}$ is the matrix whose columns are  $\bx_1, \bx_2, \ldots, \bx_p$.
\end{proposition}
\begin{proof}[of Proposition~\ref{proposition:projection-from-matrix}]
It can be easily verified that $\bH$ is symmetric and idempotent. By SVD of $\bX=\bU\bSigma\bV^\top$, we have $\bH = \bX(\bX^\top\bX)^{-1}\bX^\top = \bU\bSigma(\bSigma^\top\bSigma)^{-1} \bSigma^\top \bU^\top$. Let $\bU=[\bu_1, \bu_2, \ldots, \bu_n]$ be the column partition of $\bU$. From Theorem~\ref{theorem:svd-four-orthonormal-Basis}, $\{\bu_1, \bu_2, \ldots, \bu_p\}$ is an orthonormal basis of $\cspace(\bX)$. And $\bSigma(\bSigma^\top\bSigma)^{-1} \bSigma^\top$ in $\bH$ is an $n\times n$ matrix, where the upper-left part is a $p\times p$ identity matrix and the other parts are zero. Apply this observation of $\bH$ into spectral theorem, $\{\bu_1, \bu_2, \ldots, \bu_p\}$ is also an orthonormal basis of $\cspace(\bH)$.
Thus, it follows that $\cspace(\bH)=\cspace(\bX)$, and the orthogonal projection $\bH$ is projecting onto $\cspace(\bX)$, from which the result follows.
\end{proof}

The proposition above brings us back to the result we have shown at the beginning of this section. 
For the OLS estimate to minimize $\normtwo{\by-\bX\bbeta}^2$, which searches for an estimate  $\widehatbbeta$ so that $\widehat{\by} = \bX\widehatbbeta$ is in $\cspace(\bX)$, minimizing the distance between $\bX\widehatbbeta$ and $\by$. An orthogonal projection matrix $\bH = \bX(\bX^\top\bX)^{-1}\bX^\top$ can project $\by$ onto the column space of $\bX$, and the projected vector is $\widehat{\by} = \bH\by$. 
By Theorem~\ref{theorem:minimal-distance-orthogonal-projection}, this projection ensures the squared distance between $\widehat{\by}$ and $\by$ is minimized.

To repeat, the hat matrix $\bH$ has a clear  geometric interpretation: it drops a perpendicular to the hyperplane. Here, $\bH = \bX(\bX^\top\bX)^{-1}\bX^\top$ drops $\by$ onto the column space of $\bX$: $\widehat{\by} = \bH\by$.
Idempotency also has a geometric interpretation. Additional $\bH$'s also drop a perpendicular to the hyperplane. But it has no additional effect because we are already on that hyperplane. Therefore, $\bH^2 \by = \bH\by$. This scenario is shown in Figure~\ref{fig:ls-geometric1}. The sum of squared error is then equal to the squared Euclidean distance between $\by$ and $\widehat{\by}$. Thus, the least squares solution for $\bbeta$ corresponds to the orthogonal projection of $\by$ onto the column space of $\bX$.

\index{Column space}
\begin{lemma}[Column space of projection matrices]\label{lemma:column-space-of-projection}
We notice that the hat matrix $\bH = \bX(\bX^\top\bX)^{-1}\bX^\top$ is used to project any vector in $\real^n$ onto the column space of $\bX \in \real^{n\times p}$. In other words, $\bH\by \in \cspace(\bX)$. Notice again that  $\bH\by$ is nothing but a combination of the columns of $\bH$, thus $\cspace(\bH) = \cspace(\bX)$. 

More generally, for any projection matrix $\bH$ that projects a vector onto a subspace $\mathcalV$, then $\cspace(\bH) = \mathcalV$.
\end{lemma}
\begin{proof}[of Lemma~\ref{lemma:column-space-of-projection}]
Since $\bH = \bX(\bX^\top\bX)^{-1}\bX^\top = \bX\bC$ (let $\bC=(\bX^\top\bX)^{-1}\bX^\top$), the columns of $\bH$ are combinations of columns of $\bX$. Thus, $\cspace(\bH) \subseteq \cspace(\bX)$.	
By Lemma~\ref{lemma:rank-of-symmetric-idempotent2}, we have 
$$
\begin{aligned}
\rank(\bH) &= \trace(\bH)=\trace\big(\bX(\bX^\top\bX)^{-1}\bX^\top\big)  \\
&= \trace\big((\bX^\top\bX)^{-1}\bX^\top\bX\big) =\trace(\bI_p)=p.
\end{aligned}
$$
where the third equality follows from the fact that the trace of a product is invariant under cyclical permutations of the factors: $\trace(\bA\bB\bC) = \trace(\bB\bC\bA) = \trace(\bC\bA\bB)$.
Thus, the rank of $\bH$ equals the rank of $\bX$ such that $\cspace(\bH) = \cspace(\bX)$.
\end{proof}

\subsection{Properties of Orthogonal Projection Matrices}
Proposition~\ref{proposition:projection-from-matrix} shows that $\bH=\bX(\bX^\top\bX)^{-1}\bX^\top$ is the orthogonal projector onto $\cspace(\bX)$ if $\bX$ has full column rank, which is the hat matrix we used in the least squares solution. 
More generally, we have the following result.
\begin{theoremHigh}[Orthogonal projector onto general subspaces]\label{theorem:orthogo_genspa}
Let $ \mathcalV $ be  a subspace in $ \real^n $ with dimension $r$. Let $ \bQ_1\in\real^{n\times r} $ and $ \bQ_2 \in\real^{n\times (n-r)}$ be semi-orthogonal matrices (i.e., their columns are mutually orthonormal; Definition~\ref{definition:orthogn_mat}) such that $ \cspace(\bQ_1) = \mathcalV $ and $ \cspace(\bQ_2) = \mathcalV^{\perp} $, where $ \mathcalV^{\perp} $ denotes the orthogonal complement of $ \mathcalV $. Then the orthogonal projectors {onto} $ \mathcalV$ and $ \mathcalV^\perp$ are given by
\begin{equation}\label{equation:orthog_genspa1}
\bH_1 = \bQ_1  \bQ_1^\top
\qquad \text{and}\qquad
\bH_2 = \bQ_2\bQ_2^\top,
\end{equation}
respectively.
\end{theoremHigh}
\begin{proof}[of Theorem~\ref{theorem:orthogo_genspa}]
We have $ \bH_1^2 = \bQ_1\bQ_1^\top\bQ_1\bQ_1^\top = \bQ_1\bQ_1^\top = \bH_1 $ since $\bQ_1^\top\bQ_1=\bI_r$. This shows that $ \bH_1 $ is a projector onto $ \mathcalV $. 
Since $\bH_1$ is symmetric, this completes the proof for the first part.
The second part follows from a similar argument.
\end{proof}

As a direct consequence of the above analysis, Theorem~\ref{theorem:svd-four-orthonormal-Basis} shows four orthogonal projectors in the context of SVD of a matrix.
\begin{theoremHigh}[SVD-related orthogonal projections]\label{theorem:svd_ortho_proj}
Let $\bX=\bU\bSigma\bV^\top$ be the full SVD of $\bX\in\real^{n\times p}$ with rank $r$. Suppose we have the following column partitions
\[
\begin{blockarray}{ccc}
\begin{block}{c[cc]}
\bU=&	\bU_1 & \bU_2   \\
\end{block}
& n\times r & n\times (n-r)   \\
\end{blockarray}
,\qquad
\begin{blockarray}{ccc}
\begin{block}{c[cc]}
\bV=	&	\bV_1 & \bV_2   \\
\end{block}
& p\times r & p\times (p-r)   \\
\end{blockarray},
\]
where $\bU_1$ and $\bV_1$ consist of the first $r$ columns of $\bU$ and $\bV$, respectively.
Then  the four orthogonal projections can be obtained by 
$$
\begin{aligned}
\bV_1\bV_1^\top &= \text{projection onto $\cspace(\bX^\top)$},
\quad
&\bV_2\bV_2^\top &=\text{projection onto $\nspace(\bX)$},\\
\bU_1\bU_1^\top &= \text{projection onto $\cspace(\bX)$},
\quad
&\bU_2\bU_2^\top &= \text{projection onto $\nspace(\bX^\top)$}.\\
\end{aligned}
$$
\end{theoremHigh}

Furthermore, there are also four orthogonal projectors associated with the pseudo-inverse of a matrix.
\begin{theoremHigh}[Pseudo-inverse-related orthogonal projections]\label{theorem:orthogonal-from-pseudo-inverse}
Given any matrix $\bX \in \real^{n\times p}$ and its pseudo-inverse $\bX^+ \in \real^{p\times n}$,  the following hold: 
\begin{itemize}
\item $\bH \triangleq \bX\bX^+$ is the orthogonal projector onto the column space of $\bX$.~\footnote{When $\bX$ has full column rank, this reduces to $\bH = \bX(\bX^\top\bX)^{-1}\bX^\top$ in Proposition~\ref{proposition:projection-from-matrix}.}
\item $\bI-\bH$ is the orthogonal projector onto the null space of $\bX^\top$.
\item $\bP\triangleq \bX^+\bX$ is the orthogonal projector onto the row space of $\bX$.
\item $\bI-\bP$ is the orthogonal projector onto the null space of $\bX$.
\end{itemize}
\end{theoremHigh}
\begin{proof}[of Theorem~\ref{theorem:orthogonal-from-pseudo-inverse}]
Since  $\bH^\top = (\bX\bX^+)^\top=\bX\bX^+=\bH$ from the definition of the pseudo-inverse, we see that $\bH$ is symmetric. Moreover, $\bH$ is idempotent, which confirms that $\bH$ is an orthogonal projector.
From Table~\ref{table:different-inverses}, we conclude that $\cspace(\bH)=\cspace(\bX\bX^+)=\cspace(\bX)$. This implies that $\bH$ is the orthogonal projector onto the column space of $\bX$. Similarly, we can prove $\bP= \bX^+\bX$ is the orthogonal projector onto the row space of $\bX$.
\end{proof}

In fact, $(\bI-\bH)$ is also symmetric idempotent if $\bH$ is symmetric idempotent. 
In general, when $\bH$ projects onto a subspace $\mathcalV$, the matrix $\bI-\bH$ projects onto the orthogonal complement  $\mathcalV^\perp$. Using the notation from Theorem~\ref{theorem:orthogo_genspa}, it follows that $\bH_2\equiv \bI-\bH_1$ since $\bQ\triangleq [\bQ_1, \bQ_2]$ is orthogonal such that $\bQ\bQ^\top = \bQ_1\bQ_1^\top+\bQ_2\bQ_2^\top = \bI$.
Alternatively, we have the following result.

\begin{proposition}[Project onto $\mathcalV^\perp$]\label{proposition:orthogonal-projection}
Let $\mathcalV$ be a subspace, and let $\bH$ be an orthogonal projector onto $\mathcalV$. Then, $\bI-\bH$ is the orthogonal projection matrix onto $\mathcalV^\perp$.~\footnote{$\bI-\bH$ is called a \textit{complementary projector} of $\bH$, vice versa}

The claim can be extended further that suppose $\mathcalV_1 \subseteq \mathcalV$ and $\mathcalV_2 \subseteq \mathcalV^\perp$. Then, $\bH_1$ is the orthogonal projector that projects onto $\mathcalV_1$ and $\bH_2$ is the orthogonal projector that projects onto $\mathcalV_2$ if and only if $\bH_1\bH_2 = \bzero$.
\end{proposition}
\begin{proof}[of Proposition~\ref{proposition:orthogonal-projection}]
First, $(\bI-\bH)$ is symmetric, $(\bI-\bH)^\top = \bI - \bH^\top = \bI-\bH$ since $\bH$ is symmetric. And 
$$
(\bI-\bH)^2 = \bI^2 -\bI\bH -\bH\bI +\bH^2 = \bI-\bH.
$$
Thus, $\bI-\bH$ is an orthogonal projection matrix. By spectral theorem again, let $\bH =\bQ\bLambda\bQ^\top$. Then, $\bI-\bH = \bQ\bQ^\top - \bQ\bLambda\bQ^\top = \bQ(\bI-\bLambda)\bQ^\top$. Hence the column space of $\bI-\bH$ is spanned by the eigenvectors of $\bH$ corresponding to the zero eigenvalues of $\bH$ (by Proposition~\ref{proposition:eigenvalues-of-projection}), which coincides with $\mathcalV^\perp$.

For the second part, since $\cspace(\bH_1) = \mathcalV_1$ and $\cspace(\bH_2) = \mathcalV_2$, every column of $\bH_1$ is perpendicular to each column of $\bH_2$. Thus, $\bH_1\bH_2=\bzero$. Conversely, suppose $\bH_1\bH_2=\bzero$, then $\bH_1(\bH_2\by) = \bzero$ for all $\by$. Thus $\mathcalV_1\perp \mathcalV_2$.
\end{proof}
In particular, from the above result, it can be easily verified when $\mathcalV_1 = \mathcalV$ and $\mathcalV_2 = \mathcalV^\perp$, we have $\bH(\bI-\bH) = \bzero$. 

\index{Uniqueness}
A projection matrix that projects any vector onto a subspace is not unique in general. However, when we restrict to orthogonal projections, the corresponding projection matrix becomes unique.
\begin{proposition}[Uniqueness of orthogonal projection]\label{proposition:unique-projection-orthogonal}
If $\bP$ and $\bH$ are orthogonal projection matrices onto the same subspace $\mathcalV$, then $\bP = \bH$.
\end{proposition}
\begin{proof}[of Proposition~\ref{proposition:unique-projection-orthogonal}]
For any vector $\bv$ in $\real^n$, it can be factored into a vector $\bv_p$ in $\mathcalV$ and a vector $\bv_n$ in $\mathcalV^\perp$ such that $\bv = \bv_p+\bv_n$ and $\bv_p^\top \bv_n=0$. Then, we have
$$
\bP \bv = \bv_p = \bH\bv,
$$
such that $(\bP-\bH)\bv = \bzero$. Since any vector $\bv\in \real^n$ is in the null space of $\bP-\bH$, it follows that $\bP-\bH$ is of rank 0, and $\bP=\bH$.
\end{proof}

A direct consequence of Proposition~\ref{proposition:unique-projection-orthogonal} and Theorem~\ref{theorem:orthogo_genspa} is the following result on the uniqueness of orthogonal projectors constructed from different sets of orthonormal bases.
\begin{corollary}[Uniqueness of orthogonal projection]\label{corollary:unique-projection-orthogonal}
Let $\bU, \bQ\in\real^{n\times r}$ be semi-orthogonal matrices such that $\bU\neq \bQ$ and $\cspace(\bU)=\cspace(\bQ)$. Then the orthogonal projectors $\bH_1\triangleq \bU\bU^\top$ and $\bH_2\triangleq \bQ\bQ^\top$ are the same.
\end{corollary}
The proof is straightforward from Proposition~\ref{proposition:unique-projection-orthogonal} and Theorem~\ref{theorem:orthogo_genspa}. Alternatively, we provide a self-contained proof below.
\begin{proof}[of Corollary~\ref{corollary:unique-projection-orthogonal}]
Given that the column spaces of $ \bU $ and $ \bQ $ are the same, there exists an orthogonal matrix $ \bZ \in \real^{r \times r} $ such that $ \bU = \bQ\bZ $. 
The existence of such a matrix $\bZ$ is trivial since $\cspace(\bU)=\cspace(\bQ)$. To see that $\bZ$ is orthogonal, we have $\bI_r=\bU^\top\bU=\bZ^\top\bQ^\top\bQ\bZ=\bZ^\top\bZ$. Therefore, $\bZ$ is orthogonal.

Now, let's calculate $ \bU\bU^\top $:
$$
\bU\bU^\top = (\bQ\bZ)(\bQ\bZ)^\top = \bQ\bZ(\bZ^\top \bQ^\top) = \bQ(\bZ\bZ^\top)\bQ^\top=\bQ\bQ^\top.
$$
This completes the proof.
\end{proof}

\begin{remark}[Equivalence between SVD and QR projections]\label{remark:svd_qr_projs}
Let $\bX\in\real^{n\times p}$ have full column rank.
Suppose $[\bU_1, \bU_2]\bSigma[\bV_1, \bV_2]^\top$ and $[\bQ_1, \bQ_2]\bR$ (with $\bU_1, \bQ_1\in\real^{n\times p}$ ) are the full SVD and QR decompositions of $\bX$, respectively. 
Then, $\bU_1\bU_1^\top =\bQ_1\bQ_1^\top$ and $\bU_2\bU_2^\top =\bQ_2\bQ_2^\top$ are two set of orthogonal projectors onto $\cspace(\bX)$ and $\nspace(\bX^\top)$, respectively.
\end{remark}
\index{Nested projection}

\begin{proposition}[Nested projection]\label{proposition:nested-projection}
Let $\mathcalV_1 \subseteq \mathcalV_2 \subseteq \real^n$ be two nested linear subspaces. 
Suppose $\bH_1$ is the orthogonal projection onto $\mathcalV_1$, and $\bH_2$ is the orthogonal projection onto $\mathcalV_2$. 
Then the following properties hold:
\begin{enumerate}[(i)]
\item  $\bH_2\bH_1 = \bH_1 = \bH_1\bH_2$;
\item  $\bH_2-\bH_1$ is also an orthogonal projection.
\end{enumerate}
\end{proposition}
\begin{proof}[of Proposition~\ref{proposition:nested-projection}]
For all $\by \in \real^n$, we have $\bH_1\by \in \mathcalV_1$. This implies $\bH_1\by\in \mathcalV_1\subseteq\mathcalV_2$. Thus,
$$
\bH_2(\bH_1\by) = \bH_1\by. \qquad (\text{from Definition~\ref{definition:projection-matrix}})
$$
Then $(\bH_2\bH_1 -\bH_1)\by = \bzero$ for all $\by\in \real^n$. That is, the dimension of the null space $\nspace(\bH_2\bH_1 -\bH_1)=n$ and the rank of $\bH_2\bH_1 -\bH_1$ is 0, which results in $\bH_2\bH_1 =\bH_1$.

For $\bH_1\bH_2$, both $\bH_1$ and $\bH_2$ are symmetric such that $\bH_1\bH_2 = \bH_1^\top\bH_2^\top = (\bH_2\bH_1)^\top = \bH_1^\top = \bH_1$, which completes the proof of part (i).

To see the second part, we notice that $(\bH_2-\bH_1)^\top = \bH_2-\bH_1$ and 
$$
\begin{aligned}
(\bH_2-\bH_1)^2 &= \bH_2^2-\bH_2\bH_1 - \bH_1\bH_2 + \bH_1^2
=\bH_2-\bH_1 - \bH_1 + \bH_1 
=\bH_2 - \bH_1,
\end{aligned}
$$
which states that $\bH_2-\bH_1$ is both symmetric and idempotent. This completes the proof.
\end{proof}

\index{Positive semidefinite}
To conclude, we claim that orthogonal projection matrices are positive semidefinite (PSD).
\begin{proposition}[Symmetric projection]\label{proposition:symmetric-projection-psd}
Any orthogonal projection matrix $\bH$ is positive semidefinite.
\end{proposition}
\begin{proof}[of Proposition~\ref{proposition:symmetric-projection-psd}]
Since $\bH$ is symmetric and idempotent. For any vector $\bx$, we have 
$$
\bx^\top \bH \bx = \bx^\top \bH\bH \bx = \bx^\top \bH^\top\bH \bx=\normtwo{\bH \bx} \geq 0.
$$
Thus, $\bH$ is PSD.
\end{proof}
\index{Symmetric projection}

\subsection{Properties of Oblique Projection Matrices}\label{section:prop_obli_proj}

Proposition~\ref{proposition:projection-from-matrix} highlights the role of orthogonal projection in the standard least squares problem. This projection plays a central role in estimating the best fit solution by minimizing the sum of squared residuals. However, in more general settings where the errors are not assumed to be isotropic or homoskedastic, such as in the \textit{generalized least squares (GLS)} framework, the notion of projection becomes more nuanced. In these cases, we encounter an oblique projection, which allows for projections along directions that are not necessarily orthogonal with respect to the standard Euclidean inner product. We will explore this concept in detail in Section~\ref{section:generalizedLS}, where we discuss how GLS accounts for correlations and heteroskedasticity in the error terms through the use of a weighting matrix.

As mentioned previously, a matrix $ \bP \in \real^{n \times n} $ that satisfies $ \bP^2 = \bP $ and $ \bP^\top \neq \bP $ is an \textit{oblique projector} (Lemma~\ref{lemma:symmetric-projection-matrix}). It splits any vector $ \by \in \real^n $ into a sum $ \by = \bP\by + (\bI - \bP)\by $:
$$
\bP\by \in \cspace(\bP)
\qquad\text{and}\qquad 
\bP\by \not\perp (\bI - \bP)\by.
$$
Consider first the two-dimensional case. Let $ \bu $ and $ \bv $ be unit vectors in $ \real^2 $ such that $ \cos(\theta) = \bu^\top \bv > 0 $. 
If $\bu\neq \bv$, then 
\begin{itemize}
\item $ \bP \triangleq \bu (\bv^\top \bu)^{-1} \bv^\top = \frac{1}{\cos(\theta)} \bu\bv^\top $ is the oblique projector onto $ \bu $ along the orthogonal complement of $ \bv $. That is, $\bP\by$ is a scalar multiple of $\bu$ and $\bv^\top(\bI-\bP)\by=\bzero$ for any $\by\in\real^n$. 
\item $ \bP^\top = \bv (\bu^\top \bv)^{-1} \bu^\top= \frac{1}{\cos(\theta)}  \bv\bu^\top $ is the oblique projector onto $ \bv $ along the orthogonal complement of $ \bu $. That is, $\bP\by$ is a scalar multiple of $\bv$ and $\bu^\top(\bI-\bP)\by=\bzero$ for any $\by\in\real^n$.
\end{itemize} 
If $ \bu = \bv $, then $ \bP $ is an orthogonal projector and $ \cos (\theta) = 1 $ (see the subsection below Theorem~\ref{theorem:qr-decomposition-in-ls}, the projections used in the QR decomposition). When $ \bv $ is almost orthogonal to $ \bu $, then $ \normtwo{\bP} = 1 / \cos(\theta) $ becomes large.

We showed in Proposition~\ref{proposition:projection-from-matrix} that $\bH\triangleq\bX (\bX^\top \bX)^{-1}  \bX^\top$ is an orthogonal projector onto $\cspace(\bX)$ if $\bX\in\real^{n\times p}$ has full column rank.
Using the matrix $\bX$, we can also find an oblique projector by introducing a positive definite matrix.
\begin{proposition}[Oblique projector onto $\cspace(\bX)$]\label{proposition:obli_cspacx}
Let $\bX\in\real^{n\times p}$ have full column rank $p$ ($p\leq n$), and let $ \bOmega \neq \bI \in\real^{n\times n}$ be positive definite. Then,
\begin{equation}\label{equation:obproj_xxt1}
\bP \triangleq \bX (\bX^\top \bOmega^{-1} \bX)^{-1} \bX^\top \bOmega^{-1}, 
\end{equation}
is an oblique projector \textbf{onto} $ \cspace(\bX) $ \textbf{along}   the space $ \bOmega\cspace(\bX)^\perp $:
\begin{equation}\label{equation:obproj_xxt2}
\bP\by \in \cspace(\bX)
\qquad\text{and}\qquad 
(\bI - \bP)\by \in \bOmega\cspace(\bX)^\perp \equiv \bOmega\nspace(\bX^\top). 
\end{equation}
\end{proposition}
\begin{proof}[of Proposition~\ref{proposition:obli_cspacx}]
Let $ \mathcalC \triangleq \cspace(\bX) $, and let $ \mathcalC^\perp $ be the orthogonal complement of $ \mathcalC $.
Since $ \bP^2 = \bP $, $ \bP $ is a projector.
For any $ \by \in \real^n $, $ \bP \by \in \cspace(\bX) $.
On the other hand, if $ \by \in \cspace(\bX) $, say $ \by = \bX \bbeta $, then:
$$
\bP \by = \bX (\bX^\top \bOmega^{-1} \bX)^{-1} (\bOmega^{-1} \bX)^\top \bX \bbeta = \bX \bbeta = \by.
$$
So $ \bP $ is a projector onto $ \cspace(\bX) $.

Let $\mathcalS \triangleq \{ \bu \in \real^n \mid \bP \bu = \bzero \}$.
We want to prove $\mathcalS = \bOmega \mathcalC^\perp$.
To see this, we first suppose $ \by \in \mathcalS $, i.e., $ \bP \by = \bzero $.
Since $\bX$ has full column rank, then
$$
(\bX^\top \bOmega^{-1} \bX)^{-1} (\bOmega^{-1} \bX)^\top \by = \bzero
\qquad\implies \qquad 
(\bOmega^{-1} \bX)^\top \by = \bzero  
\iff  \bX^\top (\bOmega^{-1} \by) = \bzero.
$$
Therefore, $ \bOmega^{-1} \by \in \mathcalC^\perp \implies \by \in \bOmega \mathcalC^\perp  \implies \mathcalS\subseteq\bOmega \mathcalC^\perp$.

Conversely, suppose $ \by = \bOmega \bz $ with $ \bz \in \mathcalC^\perp $.
Then,
$
\bX^\top \bOmega^{-1} \by = \bX^\top \bz = \bzero.
$
Therefore, $\bP \by = \bzero$.
This concludes that $ \bP $ is an oblique projection onto $ \cspace(\bX) $, along $ \bOmega \cspace(\bX)^\perp $.
\end{proof}

More generally, we have the following result.
\begin{theoremHigh}[Oblique projector onto general subspaces]\label{theorem:obliq_genspa}
Let $ \mathcalV $ and $ \mathcalW $ be two complementary subspaces in $ \real^n $:
$$\mathcalV \cap \mathcalW = \bzero 
\qquad \text{and}\qquad 
\mathcalV \cup \mathcalW = \real^n.  
$$
Let $ \bU_1\in\real^{n\times r} $ and $ \bV_1 \in\real^{n\times (n-r)}$ be semi-orthogonal matrices (i.e., their columns are mutually orthonormal; Definition~\ref{definition:orthogn_mat}) such that $ \cspace(\bU_1) = \mathcalV $ and $ \cspace(\bV_1) = \mathcalW^{\perp} $, where $ \mathcalW^{\perp} $ is the orthogonal complement of $ \mathcalW $. Then the oblique projector \textbf{onto} $ \mathcalV $ \textbf{along} $ \mathcalW $ is
\begin{equation}\label{equation:obliq_genspa1}
\bP_{\mathcalV, \mathcalW} = \bU_1 (\bV_1^\top \bU_1)^{-1} \bV_1^\top. 
\end{equation}
Similarly, let $ \bU_2\in\real^{n\times (n-r)} $ and $ \bV_2\in\real^{n\times r} $ be semi-orthogonal matrices such that $ \mathcalV^{\perp} = \cspace(\bU_2) $ and $ \mathcalW = \cspace(\bV_2) $. Then,
\begin{align}
\bP_{\mathcalW, \mathcalV} = \bV_2 (\bU_2^\top \bV_2)^{-1} \bU_2^\top; \label{equation:obliq_genspa2}\\
\bP_{\mathcalV, \mathcalW} + \bP_{\mathcalW, \mathcalV} = \bI;     \label{equation:obliq_genspa3}\\
\bP_{\mathcalV, \mathcalW}^\top = \bP_{\mathcalW^{\perp}, \mathcalV^{\perp}}. \label{equation:obliq_genspa4}
\end{align}
\end{theoremHigh}
\begin{proof}[of Theorem~\ref{theorem:obliq_genspa}]
We have $ \bP_{\mathcalV, \mathcalW}^2 = \bU_1 (\bV_1^\top \bU_1)^{-1} \bV_1^\top \bU_1 (\bV_1^\top \bU_1)^{-1} \bV_1^\top = \bP_{\mathcalV, \mathcalW} $. This shows that $ \bP_{\mathcalV, \mathcalW} $ is a projector onto $ \mathcalV $.

\paragraph{Forward implication.}
Let $\mathcalS \triangleq \{ \bu \in \real^n \mid \bP_{\mathcalV, \mathcalW} \bu = \bzero \}$.
We want to prove $\mathcalS = \mathcalW \equiv \cspace(\bV_2) \equiv \cspace(\bV_1)^\perp \equiv\nspace(\bV_1^\top)$.
To see this, we first suppose $ \by \in \mathcalS $, i.e., $ \bP_{\mathcalV, \mathcalW} \by = \bzero $.
Since $\bU_1^\top\bU_1 = \bI_r$, we have
$$
\bU_1 (\bV_1^\top \bU_1)^{-1} \bV_1^\top \by = \bzero
\qquad\implies \qquad 
\bV_1^\top\by=\bzero.
$$
Therefore, $ \by \in \nspace(\bV_1^\top) \implies \by \in \mathcalW \implies \mathcalS\subseteq\mathcalW$.
\paragraph{Backward implication.}
Conversely, suppose $ \by \in\mathcalW $ such that $\bV_1^\top\by=\bzero$.
Then, $\bP_{\mathcalV, \mathcalW} \by = \bzero$.
This concludes that $ \bP_{\mathcalV, \mathcalW} $ is an oblique projection onto $ \mathcalV $, along $ \mathcalW $.

Similarly, $ \bP_{\mathcalW, \mathcalV} = \bV_2 (\bU_2^\top \bV_2)^{-1} \bU_2^\top $ is the projector onto $ \mathcalW $ along $\mathcalV$. 
To prove \eqref{equation:obliq_genspa3}, we first note that the assumption implies $ \bV_1^\top \bV_2 = \bzero $ and $ \bU_2^\top \bU_1 = \bzero $. Then,
$$
\begin{aligned}
\bP_{\mathcalV, \mathcalW} + \bP_{\mathcalW, \mathcalV} 
&= \bU_1 (\bV_1^\top \bU_1)^{-1} \bV_1^\top + \bV_2 (\bU_2^\top \bV_2)^{-1} \bU_2^\top\\
&= [\bU_1,  \bV_2] \left( 
[\bV_1, \bU_2]^\top 
[\bU_1 , \bV_2 ]
\right)^{-1} 
[\bV_1 , \bU_2]^\top \\
&= [\bU_1, \bV_2] [\bU_1, \bV_2]^{-1} [\bV_1 , \bU_2]^{-\top} [\bV_1 , \bU_2]^\top = \bI. 
\end{aligned}
$$
The equality~\eqref{equation:obliq_genspa4} follows from the expression $ \bP_{\mathcalV, \mathcalW}^\top = \bV_1 (\bU_1^\top \bV_1)^{-1} \bU_1^\top $.
This completes the proof.
\end{proof}

Suppose $ \bH \in\real^{n\times n}$ is an orthogonal projector onto a subspace $\mathcalV$ where $\dim(\mathcalV)=r$. And let $\bU_1\in\real^{n\times r}$ be a semi-orthogonal matrix whose columns form an orthonormal basis for the subspace $\mathcalV$. 
Then we have $\bH\equiv \bU_1\bU_1^\top$ (Theorem~\ref{theorem:orthogo_genspa} and Corollary~\ref{corollary:unique-projection-orthogonal}).
Since $\bU_1^\top \bU_1=\bI_r$, it holds that 
\begin{equation}\label{equation:obl_or_ineq}
\normtwo{\bv}^2 \geq \normtwo{\bH \bv}^2 = \normtwobig{\bU_1\bU_1^\top\bv}^2 =\bv^\top \bU_1\bU_1^\top \bU_1\bU_1^\top\bv  =\normtwobig{\bU_1^\top\bv}^2 ,
\; \forall \; \bv \in \real^n,
\end{equation}
where the first inequality follows by Lemma~\ref{lemma:angle-orthogonal-projection}.

\begin{exercise}
Show that the converse of the above result is also true: a projector $ \bH $ is an orthogonal projector only if \eqref{equation:obl_or_ineq} holds for all $\bv\in\real^n$. 
\textit{Hint: See the comparison between orthogonal and oblique projections, as shown in Figure~\ref{fig:ls-geometric1-compare}.}
\end{exercise}

\index{Random noise}
\section{OLS in  Geometry with Noise Disturbance}\label{sec:geometry-noise-disturbance}

We revisit the concept of the orthogonal projection matrix in the context of the least squares problem.
\begin{remark}[Important facts about hat matrix (part 1)]\label{remark:imp_hat_pro1}
Let $\bX\in\real^{n\times p0}$.
\begin{itemize}
\item 1. As we assume $\bX$ is fixed and has full rank with $n\geq p$. It is known that the rank of $\bX$ is equal to the rank of its \textit{Gram matrix}, defined as $\bX^\top\bX$, such that 
\begin{equation}
	\rank(\bX) = \rank(\bX^\top\bX). \nonumber
\end{equation}

\item 2. The rank of an orthogonal projection matrix is the dimension of the subspace onto which it projects. Hence, the rank of $\bH$ is $p$ when $\bX$ has full rank and $n\geq p$:
\begin{equation}
	\rank(\bH) = \rank(\bX(\bX^\top\bX)^{-1}\bX^\top)  = p.\nonumber
\end{equation}
\item 3. The column space of $\bH$ is identical to the column space of $\bX$.
\end{itemize}
\end{remark}

Now suppose the ideal output $\by$ comes from some true function $g(\bX) \in \cspace(\bX) $ such that the observed output $\rvy$ is modeled as 
\begin{equation}
\rvy =g(\bX)+  \bepsilon,
\end{equation}
where $\bepsilon$ represents additive noise, making $\rvy$  a random variable. 
That is, the real observation $\rvy$ is disturbed by some noise random variable.
In this case, we assume that the observed values $\by$ differ from the true function $g(\bX)=\bX\bbeta$ by additive noise. 
This situation is illustrated in Figure~\ref{fig:ls-geometric2}, which provides a geometric interpretation of the least squares problem and serves as a foundation for further developments in this book.
The main components of the model are:
\begin{itemize}
\item Vector of outputs (responses): $\by  \in \real^n$ is an $n\times 1$ vector of observations of the output variable, and $n$ is the sample size.
\item Design matrix: $\bX$ is an $n\times p$ matrix of inputs, and $p$ is the dimension of the inputs for each observation.
\item Vector of parameters: $\bbeta \in \real^p$ is a $p\times 1$ vector of regression coefficients.
\item Vector of noises: $\bepsilon \in \real^n$ is an $n\times 1$ vector of noises.
\item Vector of errors (residuals): $\be \in \real^n$ is an $n\times 1$ vector of errors. For predicted outputs $\widehat{\by}$, $\be = \by - \widehat{\by}$. 
Note that  $\be$ is different from $\bepsilon$: the former results from our linear model fit, while the latter is unobservable. In some literature, $\be$ is denoted as $\widehat{\bepsilon}$ to emphasize its relationship with $\bepsilon$.
\end{itemize}

\begin{figure}[h!]
\centering
\includegraphics[width=0.5\textwidth]{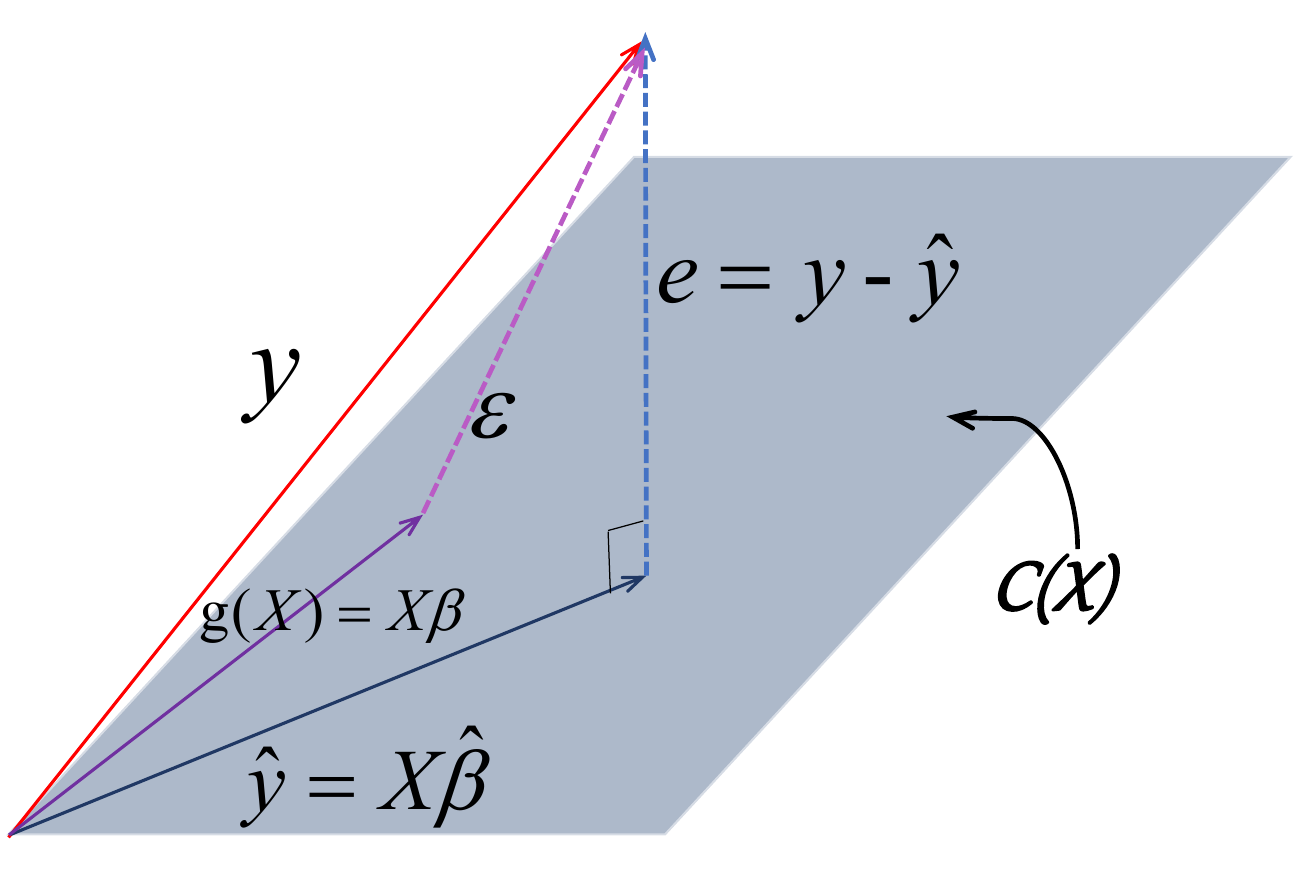}
\caption{Projection onto the hyperplane of $\cspace(\bX)$, with the output disturbed by noise $\bepsilon$.}
\label{fig:ls-geometric2}
\end{figure}

\index{Geometry interpretation}
\index{Hat matrix}
\index{Orthogonal projection}
By introducing the noise vector, we can derive additional important properties of the hat matrix:
\begin{remark}[Important facts about hat matrix (part 2)]\label{remark:pythagoras}
In light of Remark~\ref{remark:imp_hat_pro1}, we have 
\begin{itemize}
\item 4. Error vector $\be = \by-\widehat{\by} = \underline{(\bI-\bH)\by} = (\bI-\bH)(\bX\bbeta+\bepsilon) = (\bI-\bH)\bX\bbeta + (\bI-\bH)\bepsilon=\underline{(\bI-\bH)\bepsilon}$: projecting $\by$ onto the orthogonal complement of $\cspace(\bX)$ is equivalent to projecting $\bepsilon$ onto the same space. 
This result follows naturally from the geometric interpretation shown in Figure~\ref{fig:ls-geometric2}.

\item 5. The predicted output $\widehat{\by}$ and the residual vector $\be$ are orthogonal. Additionally, $\bH \bepsilon$ and $\be$ are also orthogonal.

\item 6. By the Pythagorean theorem applied to projections: $\normtwo{\by}^2 = \normtwo{\widehat{\by}}^2 + \normtwo{\be}^2$ and  $\normtwo{\bepsilon}^2 = \normtwo{\bH\bepsilon}^2 + \normtwo{\be}^2$.

\item 7. Pythagoras in general: for any orthogonal projection matrix $\bP$, we have $\normtwo{\bx}^2 = \normtwo{\bP\bx}^2 + \normtwo{(\bI-\bP)\bx}^2$.
\end{itemize}
\end{remark}
This general form of the Pythagorean identity can be verified as follows:
$$
\begin{aligned}
\normtwo{\bP\bx}^2 + \normtwo{(\bI-\bP)\bx}^2 
&= \bx^\top\bP^\top\bP\bx + \bx^\top (\bI-\bP)^\top(\bI-\bP)\bx\\
&=\bx^\top\bP\bx + \bx^\top (\bI-\bP)\bx
=\bx^\top\big[\bP\bx+(\bI-\bP)\bx\big] 
=\normtwo{\bx}^2.
\end{aligned}
$$

\section{OLS in Pseudo-Inverse}\label{section:ls_pseudo_inv}
By Theorem~\ref{theorem:rank_def_ls_prop}, when $\bX\in\real^{n\times p}$ does not necessarily have full rank, the unique least squares solution of minimum norm is characterized by 
\begin{equation}
	\be \triangleq \by - \bX\bbeta \perp \cspace(\bX) 
	\qquad \text{and} \qquad 
	\bbeta \perp \nspace(\bX).
\end{equation}
The full SVD $\bX 
= \bU \bSigma \bV^\top
%\triangleq 
%[\bU_1, \bU_2]
%\begin{bmatrixfoot}
%	\bSigma_1 & \bzero \\
%	\bzero & \bzero
%\end{bmatrixfoot}  
%[\bV_1, \bV_2]^\top
$
provides  orthogonal bases for these two subspaces (Theorem~\ref{theorem:svd-four-orthonormal-Basis}), making the SVD an ideal tool for solving least squares problems.
Consider the minimum-norm least squares problem:
$$
\min_{\bbeta \in \sB} \normtwo{\bbeta}, \quad \sB \triangleq \{ \bbeta \in \real^p \mid \normtwo{\by - \bX\bbeta} = \min \}.
$$
The problem has a unique solution that can be written as $\widehatbbeta_{\text{mn}} = \bX^+ \by$ (Theorem~\ref{theorem:rank_def_ls_prop}),
where the pseudo-inverse of $\bX$ is:
$$
\bX^+ = \bV 
\begin{bmatrix}
	\bSigma^{-1}_1 & \bzero \\
	\bzero & \bzero
\end{bmatrix} 
\bU^\top,
\quad 
\text{with }
\bSigma \triangleq 
\begin{bmatrix}
	\bSigma_1 & \bzero \\
	\bzero & \bzero
\end{bmatrix} 
\text{ and }
\bSigma_1\in\real^{r\times r}.
$$
By the uniqueness of $\bX^+$ (Lemma~\ref{lemma:uniqueness-of-pseudo-inverse}), this does not depend on the particular choice of $\bU$ and $\bV$ in the SVD.

The pseudo-inverse and the singular vectors of $\bX$ also provide simple expressions for orthogonal projections onto the four fundamental subspaces of $\bX$. 
These expressions can be verified using the Penrose conditions (see  \eqref{equation:pseudi-four-equations}, Theorems~\ref{theorem:orthogonal-from-pseudo-inverse} and \ref{theorem:svd_ortho_proj}) and the SVD:
$$
\begin{aligned}
\bP_{\cspace(\bX)} &= \bX\bX^+ = \bU_1 \bU_1^\top, \qquad &\bP_{\cspace(\bX^\top)} &= \bX^+ \bX = \bV_1 \bV_1^\top,\\
\bP_{\nspace(\bX^\top)} &= \bI - \bX\bX^+ = \bU_2 \bU_2^\top, \qquad &\bP_{\nspace(\bX)} &= \bI - \bX^+ \bX = \bV_2 \bV_2^\top,
\end{aligned}
$$
where $\bU_1 = [\bu_1, \bu_2, \ldots, \bu_r]$, $\bV_1 = [\bv_1,\bv_2, \ldots, \bv_r]$, and $r=\rank(\bX)$.

If only some of the four Penrose conditions hold, the corresponding matrix is referred to as a generalized inverse.  
Any matrix $\bX^-$ satisfying the first Penrose condition $\bX\bX^-\bX = \bX$ is called an \textit{inner inverse or $\{1\}$-inverse}. If it satisfies the second condition $\bX^-\bX\bX^- = \bX^-$, it is called an \textit{outer inverse or a $\{2\}$-inverse}.

Let $\bX^-$ be an inner-inverse of $\bX$. Then for all $\by$ such that the system $\bX\bbeta = \by$ is consistent, $\bbeta = \bX^-\by$ is a solution. The general solution can be written
\begin{equation}
\widehatbbeta = \bX^-\by + (\bI - \bX^-\bX)\balpha, \qquad \balpha\in \real^p.
\end{equation}
This form is similar to the one using the pseudo-inverse given in  \eqref{equation:unif_ls}. However, $\bX^-$ is in general not unique, and $\bX^-\by$ may not yield a minimum-norm solution of the least squares problem; see below.

\paragrapharrow{Penrose ($C1$)+($C3$), and least squares inverse.}
For any inner-inverse of $\bX$, it holds that
$$
(\bX\bX^-)^2 = \bX\bX^-\bX\bX^- = \bX\bX^-,
\qquad
(\bX^-\bX)^2 = \bX^-\bX\bX^-\bX = \bX^-\bX.
$$
This shows that both $\bX\bX^-$ and $\bX^-\bX$ are idempotent matrices, and hence (in general, oblique) projectors; see Section~\ref{section:prop_obli_proj}. 
The residual norm $\normtwo{\bX\bbeta - \by}$ is minimized when $\bbeta$ satisfies the normal equation $\bX^\top \bX\bbeta = \bX^\top \by$. Suppose that an inner-inverse $\bX^-$ also satisfies the third Penrose condition:
$
(\bX\bX^-)^\top = \bX\bX^-.
$
Then $\bX\bX^-$ becomes the orthogonal projector onto $\cspace(\bX)$, and in this case, $\bX^-$ is called a \textit{least squares inverse}. 
From the two conditions $\bX\bX^-\bX = \bX$ and $(\bX\bX^-)^\top = \bX\bX^-$, we obtain:
$$
\bX^\top = (\bX\bX^-\bX)^\top = \bX^\top \bX\bX^-.
$$
Therefore, $\bX^\top\bX\bX^-\by = \bX^\top\by$,
which shows that $\bbeta = \bX^-\by$ satisfies the normal equation and therefore is a least squares solution.

\paragrapharrow{Penrose ($C1$)+($C4$).}
A dual result also holds. 
If $\bX^-$ is an inner inverse and $(\bX^-\bX)^\top = \bX^-\bX$, then $\bX^-\bX$ is the orthogonal projector onto $\cspace(\bX^\top)$, and $\bX^-$ is called a minimum-norm inverse. If $\bX\bbeta  = \by$ is consistent, the unique solution for which $\normtwo{\bbeta}$ is smallest satisfies the normal equation of the second kind; see \eqref{equation:consis_minimunorm}:
$$
\bbeta = \bX^\top \bgamma 
\qquad\text{and}\qquad 
\bX\bX^\top \bgamma = \by.
$$
For a minimum-norm inverse $\bX^-$, we again use the identity $\bX\bX^-\bX = \bX$ and $(\bX^-\bX)^\top = \bX^-\bX$ to derive:
$$
\bX^\top = (\bX\bX^-\bX)^\top = \bX^-\bX\bX^\top.
$$
Hence, $\bbeta = \bX^\top \bgamma = \bX^-(\bX\bX^\top \bgamma) = \bX^-\by$, which shows that $\bbeta = \bX^-\by$ is the solution of minimum norm.

\paragrapharrow{Consistency guarantee.}
In the above paragraph, we assumed that the linear system is consistent. 
Consistency can be guaranteed under certain conditions using the concept of left and right inverses (Definition~\ref{definition:one_side_inverse}).
We now show that the linear system $\bX\bbeta=\by$ has a {unique} solution under specific conditions.
\begin{theoremHigh}[Unique linear system solution]\label{theorem:unique-linear-system-solution}
Let $\bX\in \real^{n\times p}$ be left-invertible ($n\geq p$), and let  $\bX_L^{-1}\in\real^{p\times n}$ denote a left inverse of $\bX$. Then, the linear system $\bX\bbeta = \by$ has a \textbf{unique} solution if and only if 
$$
(\bI_n - \bX \bX_L^{-1})\by = \bzero.
$$
In this case, the unique solution is given by
$$
\widehatbbeta = (\bX^\top\bX)^{-1}\bX^\top \by.
$$
\end{theoremHigh}
\begin{proof}[of Theorem~\ref{theorem:unique-linear-system-solution}]
Suppose $\bbeta_0$ is the solution of $\bX\bbeta =\by$, then 
$$
\begin{aligned}
\bX \bX_L^{-1} (\bX \bbeta_0) &= \bX\bX_L^{-1} \by; \\
\bX (\bX_L^{-1} \bX) \bbeta_0 &= \bX \bbeta_0 = \by.
\end{aligned}
$$
That implies $\bX\bX_L^{-1} \by=\by$ and $(\bI_n-\bX\bX_L^{-1})\by = \bzero$. 

Conversely, suppose $(\bI_n - \bX \bX_L^{-1})\by = \bzero$, and let $\bbeta_0 = \bX_L^{-1} \by$. Substituting $\bbeta_0 = \bX_L^{-1} \by$ into $(\bI_n - \bX \bX_L^{-1})\by = \bzero$, we have 
$
\bX\bbeta_0 = \by,
$
which implies $\bbeta_0 = \bX_L^{-1} \by$ is a solution of $\bX\bbeta = \by$ whenever $(\bI_n - \bX \bX_L^{-1})\by = \bzero$.

To prove the uniqueness, suppose $\bbeta_0$ and $\bbeta_1$ are two solutions of $\bX\bbeta = \by$. We have $\bX\bbeta_0=\bX\bbeta_1 = \by$, hence $\bX(\bbeta_0-\bbeta_1) = \bzero$. Since $\bX$ is left-invertible, so that $\bX$ has full column rank $p$. 
The dimension of the row space of $\bX$ is $p$ as well such that the null space of $\bX$ is of dimension 0 (i.e., $\dim(\cspace(\bX^\top)) + \dim(\nspace(\bX))=p$ by the fundamental theorem of linear algebra, see Theorem~\ref{theorem:fundamental-linear-algebra}). Therefore, $\bbeta_0=\bbeta_1$, which completes the proof.
\end{proof}

According to the fundamental theorem of linear algebra (Figure~\ref{fig:lafundamental-ls}), \textcolor{black}{if $\bX$ is left-invertible, its row space spans the entire $\real^p$} (indicating $\bX$ has full column rank $p$). The condition  $(\bI_n - \bX \bX_L^{-1})\by = \bzero$  implies that $\by$ is in the column space of $\bX$ such that $\bX\bbeta=\by$ has at least one solution, and the above proposition shows that this solution is unique.

\index{Linear system}
\index{Linear model}
\begin{theoremHigh}[Always have solution]\label{theorem:always-have-solution-right-inverse}
Suppose $\bX\in \real^{n\times p}$ is right-invertible (which implies $n\leq p$), and let $\bX_R^{-1}$ be a right inverse of $\bX$. 
Then, for any $\by\in \real^n$, the linear system $\bX\bbeta = \by$ has at least one solution, and one such solution is given by:
$$
\widehatbbeta = \bX_R^{-1}\by, 
$$
where $\bX_R^{-1}$ is a right inverse of $\bX$ and the right inverse is not necessarily unique.
\end{theoremHigh}
\begin{proof}[of Theorem~\ref{theorem:always-have-solution-right-inverse}]
It is straightforward to verify that
$
(\bX \bX_R^{-1}) \by = \bI_n \by =\by,
$
which shows that $\bX_R^{-1}\by$ is a solution of $\bX\bbeta=\by$.
\end{proof}
We observe that if $\bX$ is right-invertible, then it has full row rank $n$. According to the fundamental theorem of linear algebra  (Figure~\ref{fig:lafundamental-ls}), \textcolor{black}{the column space of $\bX$ spans the entire space of $\real^n$ if $\bX$ is right-invertible}. Hence, any vector $\by\in \real^n$ lies in the column space of $\bX$, and the system $\bX\bbeta=\by$ always has at least one solution.

\index{Generalized least squares (GLS)}
\section{Generalized Least Squares (GLS)}\label{section:generalizedLS}
We will briefly introduce the \textit{generalized least squares (GLS)} problem in this section. 
In Section~\ref{sec:bayesian-approach}, we will discuss the Bayesian approach to linear models or generalized linear models.
The Gauss-Markov extension to the GLS problem is discussed in Theorem~\ref{theorem:gauss_markov_gls}.

\paragrapharrow{Generalized least squares problem.}
We  consider the following \textit{generalized least squares (GLS)} problem
\begin{align}
	(\textbf{LS}): \qquad 	&\min_{\bbeta} (\by - \bX\bbeta)^\top  (\by - \bX\bbeta)\\
	\implies
	(\textbf{GLS}): \qquad 	&\min_{\bbeta} (\by - \bX\bbeta)^\top {\bOmega}^{-1} (\by - \bX\bbeta),\label{equation:gls_prob_loss}
\end{align}
where $\bOmega$ is positive definite. Since $\bPhi\triangleq \bOmega^{-1}$ is also positive definite, one may wonder why we use $\bOmega^{-1}$ instead of $\bPhi$. 
The reason is that $\bOmega$ has a covariance interpretation within the Gauss-Markov model; see Theorem~\ref{theorem:gauss_markov_gls}.
However, when developing numerical methods for solving the GLS problem,  the notion $\bPhi$ is frequently used; see Section~\ref{section:gls_ellipmgs}.

\index{Spectral decomposition}
Several perspectives can be taken on the GLS solution or prediction:
\begin{enumerate}[(i)]
\item The least squares solution of $ \bbeta $ is  the value $ \widehatbbeta $ that satisfies the \textit{generalized normal equation}
\begin{equation}\label{equation:gls_prob_gne}
(\textbf{GNE}):\qquad 	\bX^\top \bOmega^{-1} \bX \widehatbbeta = \bX^\top \bOmega^{-1} \by .
\end{equation}
This can be obtained, for example, using the first-order optimality condition (Proposition~\ref{proposition:fermat_fist_opt}).
\item Equivalently, the solution $ \widehatbbeta $ satisfies the \textit{orthogonality condition} 
\begin{equation}\label{equation:gls_prob_gne1}
\bX^\top \bOmega^{-1} (\by - \bX \widehatbbeta) = \bX^\top \bOmega^{-1} \be = \bzero.
\end{equation}
\item The predicted vector for $\by$ (when $\bX$ has full column rank and $\bOmega$ is PD) is given by 
\begin{equation}\label{equation:gls_prob_gne3}
\widehatby = \bX  \widehatbbeta = (\bX^\top \bOmega^{-1} \bX )^{-1}\bX^\top \bOmega^{-1} \by .
\end{equation}
This represents the oblique projection of $\by$ using the oblique projector onto $\cspace(\bX)$ along the space $\bOmega \cspace(\bX)^\perp$ (Proposition~\ref{proposition:obli_cspacx}).
\end{enumerate}
The GLS problem can be solved or interpreted through basic LS solutions in three ways:
\begin{itemize}
\item Let $\bOmega=\bQ\bLambda\bQ^\top$ be the spectral decomposition of $\bOmega$. Since $\bOmega$ is positive definite, then by Theorem~\ref{theorem:eigen_charac} the square root of $\bLambda$ exists such that $\bOmega = \bQ\bLambda^{1/2}\bLambda^{1/2}\bQ^\top\triangleq \bG\bG^\top$.
Therefore, $(\by - \bX\bbeta)^\top {\bOmega}^{-1} (\by - \bX\bbeta) = \normtwo{\bG^{-1}\by-\bG^{-1}\bX\bbeta}^2$ such that the GLS problem of $\bbeta$ is equivalent to the basic LS problem of 
\begin{equation}\label{equation:gls_asols}
\min_{\bbeta}\normtwobig{\widetildeby-\widetildebX\bbeta}^2, 
\quad \text{where }\widetildeby\triangleq \bG^{-1}\by, \ \widetildebX\triangleq \bG^{-1}\bX.
\end{equation}
\item $\bOmega$ can also be uniquely denoted as $\bOmega = \bM^2 \triangleq (\bQ\bLambda^{1/2}\bQ^\top) (\bQ  \bLambda^{1/2}\bQ^\top)$, where $\bM$ is also positive definite (Theorem~\ref{theorem:unique-factor-pd}). Thus, the GLS problem can be written as another LS problem of the form \eqref{equation:gls_asols}, where $\widetildeby\triangleq \bM^{-1}\by$ and $\widetildebX\triangleq \bM^{-1}\bX$.
\item $\bOmega$ admits the Cholesky decomposition $\bOmega = \bL\bL^\top$. 
Thus, the GLS problem can be written as another LS problem of the same form \eqref{equation:gls_asols}, where $\widetildeby\triangleq \bL^{-1}\by$ and $\widetildebX\triangleq \bL^{-1}\bX$.
\end{itemize}
Among these three approaches, the third one---based on the Cholesky decomposition---is generally the most straightforward.
The GLS problem can be solved by first computing $\bOmega = \bL\bL^\top$ and then solving the transformed systems $\bL \widetildebX = \bX$ and $\bL\widetildeby = \by$. The normal equation $\widetildebX^\top \widetildebX \bbeta = \widetildebX^\top \widetildeby$ are formed and solved by Cholesky factorization; see Section~\ref{section:ls_cholesky}. Alternatively, one may apply QR factorization to the transformed design matrix:
\begin{equation}\label{equation:qr_gls_raw}
	\bL^{-1} \bX = \bQ \begin{bmatrix} \bR_1\\ \bzero \end{bmatrix}, \quad \bQ = \begin{bmatrix} \bQ_1 & \bQ_2 \end{bmatrix},
\end{equation}
which leads to the solution: $\widehatbbeta = \bR_1^{-1} \bQ_2^\top \bL^{-1} \by$; see Section~\ref{section:ls_qr_gen}.

Computing the Cholesky factorization $\bOmega = \bL\bL^\top$ requires approximately $\sim n^3/3$ floating-point operations (flops) for a dense matrix $\bOmega$. 
Forming the transformed matrices $\widetildebX = \bL^{-1} \bX$ and $\widetildeby = \bL^{-1} \by$ requires a further $\sim n^2 p $ flops. This may be prohibitive unless $\bOmega$ has a favorable structure. When $\bOmega$ is a banded matrix with small bandwidth $w$, the cost of the Cholesky factorization reduces to roughly $\sim n w(w + 3)$ flops. 

\index{Weighted least squares}
\index{WLS}
\index{GLS}
\index{Generalized least squares}
\paragrapharrow{Weighted least squares (WLS).} The \textit{weighted least squares} problem is a special GLS problem in which  $\bOmega$ is chosen to be   a diagonal matrix. When the diagonal matrix is the identity matrix, WLS reduces to the ordinary least squares problem. See also Problem~\ref{problem:wei_ls}.

\index{Minimum-norm solution}
\paragrapharrow{Generalized minimum-norm problem.}
Consider a consistent linear system $\bX^\top \balpha = \bz$ of full row rank, the \textit{generalized minimum-norm (GMN) problem} is
\begin{equation}\label{equation:gmn_prob}
(\textbf{GMN}):\qquad 	\min_{\balpha} \balpha^\top \bOmega \balpha \quad \text{s.t.} \quad \bX^\top \balpha = \bz.
\end{equation}
The corresponding \textit{generalized normal equation} of the second kind is
\begin{equation}\label{equation:gls_prob_gne2}
(\textbf{GNE2}):\qquad 	\bX^\top \bOmega^{-1} \bX \bgamma = \bz, \qquad \balpha = \bOmega^{-1} \bX \bgamma.
\end{equation}
If $\bOmega = \bL\bL^\top$ is the Cholesky factorization, then $\balpha^\top \bOmega \balpha = \normtwo{\bL^\top \balpha}^2$. Hence problem \eqref{equation:gmn_prob} is equivalent to seeking the minimum-norm solution of the system
$$
\widetildebX^\top \widetildebalpha = \bz \quad \text{with}\quad\widetildebX \triangleq \bL^{-1} \bX, \ \widetildebalpha \triangleq \bL^\top \balpha.
$$
Alternatively, using the QR factorization gives $\balpha = \bL^{-\top} \bQ_1 (\bR^{-\top} \bz)$; see Section~\ref{section:ls_qr_gen}.

Similar to the augmented LS problem \eqref{equation:sys_aug_sys}, 
Problems GLS and GMN are special cases of the generalized augmented  LS problem:
\begin{equation}\label{equation:gls_aug_sys}
\text{(GAuLS)}:\qquad 
\bF \begin{bmatrix} \balpha \\ \bbeta \end{bmatrix} 
\triangleq \begin{bmatrix} \bOmega & \bX \\ \bX^\top & \bzero \end{bmatrix} \begin{bmatrix} \balpha \\ \bbeta \end{bmatrix} = \begin{bmatrix} \by \\ \bz \end{bmatrix}
\quad \by \in \real^n, \quad \bz \in \real^p.
\end{equation}
This system matrix $\bF$ is nonsingular if and only if $\rank(\bX) = p$ and
$$
\cspace(\bOmega) \cap \cspace(\bX^\top) = \{\bzero\}.
$$
In fact, if $\bOmega$ is positive definite, then it follows that the matrix $\bF \in \real^{(n+p) \times (n+p)}$ of system \eqref{equation:gls_aug_sys} has $n$ positive and $p$ negative eigenvalues \citep{bjorck2024numerical}. For this reason, \eqref{equation:gls_aug_sys} is called a saddle point system. Eliminating $\balpha$ in \eqref{equation:gls_aug_sys} gives the generalized normal equation for $\bbeta$,
\begin{equation}
\bX^\top \bOmega^{-1} \bX \bbeta = \bX^\top \bOmega^{-1} \by - \bz.
\end{equation}

An explicit expression for the inverse of augmented matrix $ \bF $ is obtained from the Schur complement of $\bOmega$ in $\bF$; see, for example, \citet{lu2021numerical}:
\begin{equation}
\bF^{-1} = \begin{bmatrix} \bOmega & \bX \\ \bX^\top & \bzero \end{bmatrix}^{-1} = \begin{bmatrix} \bOmega^{-1} (\bI - \bT) & \bOmega^{-1} \bX \bS^{-1} \\ \bS^{-1} \bX^\top \bOmega^{-1} & -\bS^{-1} \end{bmatrix},
\end{equation}
where
$$
\bS \triangleq \bX^\top \bOmega^{-1} \bX, \qquad \bT \triangleq \bX \bS^{-1} (\bOmega^{-1} \bX)^\top.
$$
In terms of the QR factorization \eqref{equation:qr_gls_raw}, the inverse is
\begin{equation}
\bF^{-1} = \begin{bmatrix} \bL^{-\top} \bQ_2 \bQ_2^\top \bL^{-1} & \bL^{-\top} \bQ_1 \bR^{-\top} \\ \bR^{-1} \bQ_1^\top \bL^{-1} & -\bR^{-1} \bR^{-\top} \end{bmatrix}.
\end{equation}

\index{Orthogonal projection}
\index{Generalized orthogonal projection}
\subsection{Orthogonal Projection in GLS}
We previously showed that the orthogonal projection matrix
$$
\bH = \bX(\bX^\top\bX)^{-1}\bX^\top, \qquad (\text{project onto $\mathcalV=\cspace(\bX)$})
$$ 
is an orthogonal projection onto the subspace $\mathcalV=\cspace(\bX)$.
Let $\widetildebX \triangleq \bOmega^{-1/2}\bX$, 
We also observe that
$$
\begin{aligned}
	\bH_{\bOmega}&=\widetildebX(\widetildebX^\top\widetildebX)^{-1}\widetildebX^\top \\
	&= \bOmega^{-1/2}\underbrace{\bX(\bX^\top  \bOmega^{-1} \bX)^{-1}\bX^\top\bOmega^{-1}}_{\triangleq \bH_2} , \qquad (\text{project onto $\mathcalV_{\bOmega} = \cspace(\bOmega^{-1/2}\bX)$})\\
	&=\bOmega^{-1/2}\bH_2
\end{aligned}
$$ 
is the orthogonal projection onto the subspace $ \mathcalV_{\bOmega} = \cspace(\bOmega^{-1/2}\bX)$. 
That is, the prediction $\widehatby$ becomes $\widehatby = \bH_{\bOmega}\by$ in the GLS problem (and $\widehatby = \bH\by$ in the LS case).
This also implies 
$$
\bH_2=\bX(\bX^\top  \bOmega^{-1} \bX)^{-1}\bX^\top\bOmega^{-1} , \qquad (\text{project onto $\mathcalV=\cspace(\bX)$})
$$ 
is also an orthogonal projection onto the subspace $\mathcalV = \cspace(\bX)$. 

Before discussing the properties of orthogonal projections in the context of generalized least squares, it is important to clarify the concept of ``symmetry."  
While a matrix $\bA$ is considered symmetric if $\bA = \bA^\top$, a more general definition of symmetry arises when considering inner products.

\index{Inner product}
To see this, the definition of an inner product must satisfy three fundamental properties:
\begin{definition}[Inner product]\label{definition:inner_prod}
In most cases, a vector norm can be derived from the \textit{inner product} of vectors (the inner product of vectors $\bx,\by\in\real^n$ is given by $\innerproduct{\bx,\by}$), which satisfies the following three conditions:
\begin{itemize}
\item \textit{Commutativity}. $\innerproduct{\bx,\by} = \langle\by,\bx\rangle$ for any $\bx,\by\in \real^n$. 
\item \textit{Linearity}. $\langle\lambda_1 \bx_1+\lambda_2\bx_2, \by\rangle = \lambda_1\innerproduct{\bx,\by}+\lambda_2\langle\bx_2,\by\rangle$ for any $\lambda_1,\lambda_2 \in \real$ and $\bx,\by\in \real^n$.
\item \textit{Positive definiteness}. $\langle \bx,\bx \rangle \geq 0$ for any $\bx \in \real^n$, and $\langle \bx,\bx \rangle = 0$ if and only if $\bx=\bzero$.
\end{itemize}
\end{definition}

\index{Symmetry}
\index{Symmetric in terms of inner product}
\index{Generalized norm}
\index{Generalized orthogonal projection}
Using this notion of an inner product, we can now define a more general form of symmetry:
\begin{definition}[Symmetric in terms of inner product]\label{definition:symmetric-inn}
A matrix $\bA \in \real^{n\times n}$ is said to be \textit{symmetric} (with respect to an inner product) if for all $\bx,\by \in \real^n$, such that $\langle\bx, \bA\by\rangle = \langle\bA\bx,\by\rangle$. 
Some specific examples include:
\begin{enumerate}
\item  For the standard inner product $\innerproduct{\bx,\by} = \bx^\top\by$, symmetric $\bA$ means $\bA = \bA^\top$.
\item  Given a positive definite matrix $\bZ$, consider the inner product $\innerproduct{\bx,\by}_{\bZ} = \bx^\top \bZ \by$; see Problem~\ref{prob:inner_zprod}. Symmetric $\bA$ means $\bx^\top \bZ \bA\by = \bx^\top \bA^\top \bZ\by$. 
This leads to the definition of the  \textit{generalized norm}, defined as   $\norm{\bx}^2_{\bZ} = \innerproduct{\bx,\bx}_{\bZ}$, also known as the $\bZ$-norm. When $\bZ=\bI$, this reduces to the standard inner product, standard symmetry, and   standard $\ell_2$ norm, respectively.
\end{enumerate}
\end{definition}

With this generalized definition of symmetry based on inner products, we can extend the idea of orthogonal projection accordingly. Specifically, we define a projection to be orthogonal in terms of the given inner product, which results in the concept of a \textit{generalized orthogonal projection}.
Then, we can check that $\bH_2=\bX(\bX^\top  \bOmega^{-1} \bX)^{-1}\bX^\top\bOmega^{-1}$ is symmetric in terms of $\innerproduct{\bx,\by}_{\bOmega^{-1}}$ such that 
$$
\bx^\top \bOmega^{-1} \bH_2\by = \bx^\top \bH_2^\top \bOmega^{-1}\by,
$$
for all $\bx,\by \in \real^n$.

\index{Generalized orthogonal projection}
\begin{lemma}[Genralized orthogonal projection in GLS]\label{lemma:orthogonal-in-glm}
Consider the GLS problem 
$$
\min_{\bbeta} (\by - \bX\bbeta)^\top {\bOmega}^{-1} (\by - \bX\bbeta),
$$
where $\bOmega$ is fixed and positive definite,  $\bX \in \real^{n\times p}$ is fixed and has full rank with $n\geq p$ (i.e., rank is $p$).~\footnote{This is equivalent to assume  $\rvy = \bX\bbeta + \bepsilon$, where $\bepsilon \sim \normal(\bzero, \sigma^2 \textcolor{black}{\bOmega})$ in the Gauss-Markov linear model; see Chapter~\ref{sec:lr-gaussian-noise}.} 
Then, 
$$
\bH_2=\bX(\bX^\top  \bOmega^{-1} \bX)^{-1}\bX^\top\bOmega^{-1}
$$
is a generalized  orthogonal projection that projects onto the column space of $\bX$.
\end{lemma}
\begin{proof}[of Lemma~\ref{lemma:orthogonal-in-glm}]
We prove $\bH_2$ is an orthogonal projection by showing that it is idempotent and symmetric, and it projects onto the column space of $\bX$. 
It can be readily verified that 
$$
\bH_2^2 = \bX(\bX^\top  \bOmega^{-1} \bX)^{-1}\bX^\top\bOmega^{-1}\bX(\bX^\top  \bOmega^{-1} \bX)^{-1}\bX^\top\bOmega^{-1} = \bH_2, 
$$
and 
$$
\begin{aligned}
\langle\bx, \bH_2\by\rangle_{\bOmega^{-1}} &= \bx^\top \bOmega^{-1} \bH_2\by 
= \bx^\top \bOmega^{-1} \left(\bX(\bX^\top  \bOmega^{-1} \bX)^{-1}\bX^\top\bOmega^{-1}\right)\by\\
&=\bx^\top \bH_2^\top \bOmega^{-1}\by
= \langle\bH_2\bx,\by\rangle_{\bOmega^{-1}}.
\end{aligned}
$$
Since $\bH_2=\bX(\bX^\top  \bOmega^{-1} \bX)^{-1}\bX^\top\bOmega^{-1} \triangleq \bX\bC$, the columns of $\bH_2$ are combinations of the columns of $\bX$, thus $\cspace(\bH_2) \subseteq \cspace(\bX)$.
By Lemma~\ref{lemma:rank-of-symmetric-idempotent2}, we have 
$$
\begin{aligned}
\rank(\bH_2) 
&= \trace(\bH_2) = \trace\big(\bX(\bX^\top  \bOmega^{-1} \bX)^{-1}\bX^\top\bOmega^{-1}\big) \\
&= \trace\big((\bX^\top  \bOmega^{-1} \bX)^{-1}\bX^\top\bOmega^{-1}\bX\big)=\trace(\bI_p)
=p,
\end{aligned}
$$
where the third equality follows from the fact that the trace of a product is invariant under cyclical permutations of the factors.
Thus, the rank of $\bH_2$ equals the rank of $\bX$. We conclude that $\cspace(\bH_2) = \cspace(\bX)$.
Therefore, $\bH_2$ is a generalized orthogonal projection onto the column space of $\bX$.
\end{proof}

The result on minimum distance in orthogonal projection (Theorem~\ref{theorem:minimal-distance-orthogonal-projection}) can also be extended to the generalized case under the $\bOmega^{-1}$-norm.
\begin{theoremHigh}[Minimum distance in generalized orthogonal projection]\label{theorem:minimal_dist_orhoproj_generalized}
Let $\mathcalV$ be a subspace of $\real^n$, and let $\bH_2$ be a generalized orthogonal projection with respect to $\bOmega^{-1}$-inner product onto $\mathcalV$. 
Then, it follows that  
$$
\norm{\by - \bH_2\by}^2_{\bOmega^{-1}} \leq \norm{\by - \bv}^2_{\bOmega^{-1}}, \qquad \forall\, \bv \in \mathcalV.
$$
\end{theoremHigh}
\index{Spectral decomposition}
\begin{proof}[of Theorem~\ref{theorem:minimal_dist_orhoproj_generalized}]
We find that $\bOmega^{-1}\bH_2$, $\bH_2\bOmega$, and $\bOmega^{-1/2}\bH_2\bOmega^{1/2}$ are all symmetric (with respect to the first inner product in Definition~\ref{definition:symmetric-inn}, i.e., the standard symmetry). 
Therefore, they all admit a  spectral decomposition. By trial and error, the right way is to decompose 
$$
\bOmega^{-1/2}\bH_2\bOmega^{1/2} = \bOmega^{-1/2}\bX(\bX^\top  \bOmega^{-1} \bX)^{-1}\bX^\top\bOmega^{-1/2}= \bQ\bLambda\bQ^\top,
$$	
where $\bQ=[\bq_1, \bq_2, \ldots, \bq_n]$ is the column partition of $\bQ$, and $\bLambda=\diag(\lambda_1, \lambda_2, \ldots, \lambda_n)$ contains the eigenvalues of $(\bOmega^{-1/2}\bH_2\bOmega^{1/2})$. Let $\dim(\mathcalV) = r$. 
We notice that $\bOmega^{-1/2}\bH_2\bOmega^{1/2}$ is symmetric and  idempotent such that its only possible eigenvalues  are 1 and 0 by Lemma~\ref{proposition:eigenvalues-of-projection}. Without loss of generality, let $\lambda_1=\lambda_2=\ldots=\lambda_r=1$ and $\lambda_{r+1}=\lambda_{r+2}=\ldots=\lambda_n=0$. 
Then, we have
\begin{itemize}
\item  $\{\bq_1, \bq_2, \ldots, \bq_n\}$ is an orthonormal basis of $\real^n$.
\item $\{\bq_1, \bq_2, \ldots, \bq_r\}$ is an orthonormal basis of $\bOmega^{-1/2}\mathcalV$, which is the subspace $\mathcalV$ rotated by $\bOmega^{-1/2}$. So for any vector $\bv\in \mathcalV$, let $\textcolor{mydarkblue}{\ba \triangleq \bOmega^{-1/2}\bv}$, we have $\ba^\top\bq_i=0$ for $i\in \{r+1, r+2, \ldots, n\}$. 
\end{itemize}
Again, let $\textcolor{mydarkblue}{\bz \triangleq \bOmega^{-1/2}\by}$. 
Since $\bOmega^{-1/2}\bH_2 = \bQ\bLambda\bQ^\top \bOmega^{-1/2}$
$$
\begin{aligned}
\norm{\by - \bH_2\by}_{\textcolor{mydarkblue}{\bOmega^{-1}}}^2 
&= \normtwo{\bQ^\top\bOmega^{-1/2}\by - \bQ^\top\bOmega^{-1/2}\bH_2\by}^2 
=\normtwo{\bQ^\top\bOmega^{-1/2}\by - \bQ^\top\textcolor{black}{\bQ\bLambda\bQ^\top \bOmega^{-1/2}}\by}^2\\
&=\normtwo{\bQ^\top\bz - \bLambda\bQ^\top \bz}^2
=\sum_{i=1}^{n}(\bz^\top \bq_i - \lambda_i\bz^\top \bq_i)^2 
=0+\sum_{i=r+1}^{n}(\bz^\top \bq_i)^2  \\
%&\leq \sum_{i=1}^{r}(\bOmega^{-1/2}\by^\top\bq_i - \bOmega^{-1/2}\bv^\top\bq_i)^2+ \sum_{i=r+1}^{n}(\bz^\top \bq_i)^2 \\
&\leq \sum_{i=1}^{r}(\bz^\top\bq_i - \ba^\top\bq_i)^2+ \sum_{i=r+1}^{n}(\bz^\top \bq_i)^2 =\normtwo{\bQ^\top \bz - \bQ^\top \ba}^2  \\
&=\normtwo{\bz - \ba}^2 = \normtwo{\bOmega^{-1/2}\by - \bOmega^{-1/2}\bv}^2 = \norm{\by-\bv}_{\textcolor{mydarkblue}{\bOmega^{-1}}}^2,
\end{aligned}
$$
which completes the proof.
\end{proof}

In previous discussions, for any orthogonal projection matrix $\bH$, we established Pythagoras' theorem and the orthogonal property in ordinary least squares (i.e., with inner product denoted by 

In previous discussions, we established Pythagoras' theorem and the orthogonal property for any orthogonal projection matrix $\bH$ in the context of ordinary least squares, where the inner product is defined as
$\innerproduct{\bx,\by} = \bx^\top\by$.
These properties are expressed as:
$$
\begin{aligned}
\normtwo{\by}^2 &= \normtwo{\bH\by}^2 + \normtwo{(\bI-\bH)\by}^2; \\
\bzero &= (\bH\by)^\top \left((\bI-\bH)\by\right).
\end{aligned}
$$
Analogously, in the generalized orthogonal projection, we have 
$$
\begin{aligned}
\norm{\by}^2_{\bOmega^{-1}} &= \norm{\bH_2\by}^2_{\bOmega^{-1}} + \norm{(\bI-\bH_2)\by}^2_{\bOmega^{-1}} ;\\
\by^\top \bOmega^{-1}\by &= \by^\top (\bH_2^\top \bOmega^{-1}\bH_2)\by + \by^\top (\bOmega^{-1}-\bH_2^\top \bOmega^{-1}\bH_2)\by,
\end{aligned}
$$
and 
$$
\begin{aligned}
\bzero &= \langle\bH_2\by,(\bI-\bH_2)\by \rangle_{\bOmega^{-1}}
= (\bH_2\by)^\top \bOmega^{-1}\left((\bI-\bH_2)\by\right).
\end{aligned}
$$

\index{WLS}
\index{Weighted least squares}
\index{Generalized least squares (GLS)}
\subsection{Equivalence between OLS and GLS}

In Theorem~\ref{theorem:minimal-distance-orthogonal-projection}, we prove that for any orthogonal projection $\bH$, the inequality $\normtwo{\by - \bH\by}^2 \leq \normtwo{\by - \bv}^2$ holds for all $\bv\; \in \mathcalV\triangleq\cspace(\bX)$. In the context of generalized least squares, the analogous result holds for any generalized orthogonal projection matrix  $\bH_2$: $\norm{\by - \bH_2\by}^2_{\bOmega^{-1}} \leq \norm{\by - \bv}^2_{\bOmega^{-1}}, \,\,\forall\, \bv \in \mathcalV\triangleq\cspace(\bX)$, where   $\bOmega$ is positive definite. 
Although OLS and GLS generally yield different estimates, we now show that under certain conditions, these two estimators are in fact equivalent.
\index{Equivalence between OLS and GLS}
\begin{theoremHigh}[Equivalence between OLS and GLS]\label{theorem:equivalence-ols-gls}
Let $\bX\in\real^{n\times p}$ with full rank $p$, and let $\bOmega$ be a positive definite matrix.
The ordinary least squares estimate $\widehatbbeta = (\bX^\top\bX)^{-1}\bX^\top \by$ is equivalent to the GLS estimate $\widetildebbeta =(\bX^\top  \bOmega^{-1} \bX)^{-1}\bX^\top\bOmega^{-1}\by$ if and only if 
$$
\cspace(\bOmega^{-1} \bX) = \cspace(\bX).
$$
\end{theoremHigh}
\begin{proof}[of Theorem~\ref{theorem:equivalence-ols-gls}]
Suppose $\widehatbbeta = \widetildebbeta$. Then for all $\by \in \real^n$, we must have 
$$
(\bX^\top\bX)^{-1}\bX^\top \by = (\bX^\top  \bOmega^{-1} \bX)^{-1}\bX^\top\bOmega^{-1}\by,
$$
which implies 
$$
(\bX^\top\bX)^{-1}\bX^\top = (\bX^\top  \bOmega^{-1} \bX)^{-1}\bX^\top\bOmega^{-1}.
$$
Taking the transpose of both sides gives 
$$
\bX (\bX^\top\bX)^{-1} = \bOmega^{-1} \bX (\bX^\top  \bOmega^{-1} \bX)^{-1}.
$$
Since $(\bX^\top\bX)^{-1} $ and $(\bX^\top  \bOmega^{-1} \bX)^{-1}$ are nonsingular matrices, this transformation reflects a change of basis. Therefore, it follows that $\cspace(\bOmega^{-1} \bX) = \cspace(\bX)$.

Conversely, suppose $\cspace(\bOmega^{-1} \bX) = \cspace(\bX)$, there must be a nonsingular matrix $\bA$ such that $\bOmega^{-1} \bX = \bX\bA$ (columns of $\bOmega^{-1} \bX$ are combinations of columns of $\bX$, and the combinations are given by $\bA$). That is $\bX = \bOmega\bX\bA$. Then, we have
$$
\begin{aligned}
(\bX^\top  \bOmega^{-1} \bX)^{-1}\bX^\top\bOmega^{-1}\by 
&=\big(\textcolor{mydarkblue}{(\bOmega\bX\bA)}^\top  \bOmega^{-1} \bX\big)^{-1}\textcolor{mydarkblue}{(\bOmega\bX\bA)}^\top\bOmega^{-1}\by \\
&=\big(\bA^\top\bX^\top\bOmega  \bOmega^{-1} \bX\big)^{-1} \bA^\top\bX^\top\bOmega\bOmega^{-1}\by \\
&=(\bX^\top \bX)^{-1} \bX^\top\by,
\end{aligned} 
$$
which completes the proof.
\end{proof}
%Since we have proved $\cspace(\bOmega^{-1} \bX) = \cspace(\bX)$, the OLS estimator is equivalent to GLS.

\section{Total Least Squares (TLS) and Other Issues}\label{section:totalles_otheriss}
%\paragrapharrow{Rank-deficiency.}
In this discussion, we assume $\bX\in \real^{n\times p}$ has full rank with $n\geq p$, ensuring that $\bX^\top\bX$ is invertible. 
However, if two or more columns of $\bX$ are perfectly correlated, the matrix $\bX$ becomes deficient, and $\bX^\top\bX$ becomes singular. 
To address this issue, one can choose $\bbeta$ that minimizes $\normtwobig{\widehatbbeta}^2$ while satisfying the normal equation. 
That is, we select the least squares solution with the smallest magnitude. 
In Sections~\ref{section:utv_ls} and~\ref{section:ls-via-svd}, we briefly discuss how to use UTV decomposition and SVD to address this rank-deficient least squares problems.

\index{Contion number}
\index{Tikhonov regularization}
\index{$\ell_2$ regularization}
\paragrapharrow{Regularizations and stability.}
However,  a common  problem that arise in the ordinary least square solution is the near-singularity of $\bX$.
Let the full SVD of $\bX$ be $\bX=\bU\bSigma\bV^\top\in\real^{n\times p}$, where $\bU\in\real^{n\times n}$ and $\bV\in\real^{p\times p}$ are orthogonal, and the main diagonal of $\bSigma\in\real^{n\times p}$ contains the singular values. Consequently, $\bX^\top\bX = \bV(\bSigma^\top\bSigma)\bV^\top \triangleq \bV\bS\bV^\top$, where $\bS\triangleq \bSigma^\top\bSigma  = \diag([\sigma_1^2, \sigma_2^2, \ldots,\sigma_p^2])\in\real^{p\times p}$ contains the squared singular values of $\bX$. When $\bX$ is nearly singular, $\sigma_p^2\approx 0$, making the inverse operation $(\bX^\top\bX)^{-1} = \bV\bS^{-1}\bV^\top$ numerically unstable. 
As a result, the solution $\widehatbbeta_{\text{LS}} =(\bX^\top\bX)^{-1}\bX^\top\by $ may diverge.
To address this issue, we typically  add an $\ell_2$ regularization term to obtain the solution for the following optimization problem:
\begin{equation}
\widehatbbeta_{\text{Tik}} = \mathop{\argmin}_{\bbeta} \normtwo{\by-\bX\bbeta}^2 +\lambda\normtwo{\bbeta}^2.
\end{equation}
This method is known as the  \textit{Tikhonov regularization method} (or simply the $\ell_2$ regularized method) \citep{tikhonov1963solution}.
The gradient of the problem is $2(\bX^\top\bX+\lambda\bI)\bbeta-2\bX^\top\by$. Thus, the least squares solution is given by 
$
\widehatbbeta_{\text{Tik}} = (\bX^\top\bX+\lambda\bI)^{-1}\bX^\top\by.
$
The inverse operation becomes $(\bX^\top\bX+\lambda\bI)^{-1} = \bV(\bS+\lambda\bI)^{-1}\bV^\top$, where $\widetildebS\triangleq(\bS+\lambda\bI)=\diag(\sigma_1^2+\lambda, \sigma_2^2+\lambda, \ldots,\sigma_p^2+\lambda)$. 
The solutions for OLS and Tikhonov regularized LS are given, respectively, by 
\begin{align}
\widehatbbeta_{\text{LS}} &= (\bX^\top\bX)^{-1}\bX^\top\by = \bV\left(\bS^{-1}\bSigma\right)\bU^\top\by;\\
\widehatbbeta_{\text{Tik}} &= (\bX^\top\bX+\lambda\bI)^{-1}\bX^\top\by = \bV\left((\bS+\lambda\bI)^{-1}\bSigma\right)\bU^\top\by,\label{equation:tik_normal_equ}
\end{align}
where the main diagonals of $\left(\bS^{-1}\bSigma\right)$ are $\diag(\frac{1}{\sigma_1}, \frac{1}{\sigma_2}, \ldots, \frac{1}{\sigma_p})$; and the main diagonals of $\left((\bS+\lambda\bI)^{-1}\bSigma\right)$ are $\diag(\frac{\sigma_1}{\sigma_1^2+\lambda}, \frac{\sigma_2}{\sigma_2^2+\lambda}, \ldots, \frac{\sigma_p}{\sigma_p^2+\lambda})$. The latter solution is more stable if $\lambda$ is greater than the   smallest nonzero squared singular value.
The condition number becomes smaller if  the smallest singular value $\sigma_p$ is close to zero:
$$
\kappa(\bX^\top\bX) = \frac{\sigma_1^2}{\sigma_p^2}
\qquad \rightarrow \qquad
\kappa(\bX^\top\bX+\lambda\bI) = \frac{\lambda+\sigma_1^2}{\lambda+\sigma_p^2}.
$$
Tikhonov regularization effectively prevents  divergence  in the least squares solution 
$\widehatbbeta_{\text{LS}} = (\bX^\top\bX)^{-1} \bX^\top \by$ when the matrix $\bX$ is nearly singular or even rank-deficient. This improvement enhances the convergence properties of both the LS algorithm and its variants, such as alternating least squares,  while addressing identifiability issues  in various settings \citep{zhang2017matrix}. As a result, Tikhonov regularization has become a widely applied technique.

\subsection{Different Least Squares Problems}

In standard linear regression, the ordinary least squares  method assumes that errors occur only in the response vector $\by$, while the data matrix $\bX$ is considered exact. However, in many real-world applications, both the data matrix $\bX$ and the response $\by$ may be subject to measurement errors.
We then discuss different forms of least squares problems.
\index{Data least squares}
\index{Total least squares}
\paragrapharrow{Data least squares.}
The least squares problem can be viewed as an optimization problem of the following form:
\begin{equation}
\widehatbbeta_{\text{LS}},\widetildeby_{\text{LS}} =\argmin_{\bbeta, \widetildeby} \normtwo{\widetildeby}^2 \quad \text{s.t. }\quad \by+\widetildeby \in \cspace(\bX),
\end{equation}
where $\widetildeby$ represents a perturbation of $\by$, i.e., a noise in the output variables.
While the OLS method accounts for errors in the response variable $\by$, the \textit{data least sqaures (DLS)} method considers errors in the predictor variables:
\begin{equation}
\widehatbbeta_{\text{DLS}}, \widetildebX_{\text{DLS}} = \mathop{\argmin}_{\bbeta, \widetildebX} \frac{1}{2}\normfbig{\widetildebX}^2 \quad \text{s.t.}\quad \by\in\cspace(\bX+\widetildebX),
\end{equation}
where $\widetildebX$ represents a perturbation of $\bX$ (i.e., a noise in the predictor variables).
That is, $(\bX+\widetildebX) \widehatbbeta_{\text{DLS}} = \by$, assuming the measured response $\by$ is noise-free.
The Lagrangian function and its gradient w.r.t. $\bbeta$ are, respectively, given by
$$
\begin{aligned}
	L(\bbeta, \widetildebX, \blambda) &= \frac{1}{2}\trace(\widetildebX\widetildebX^\top) +\blambda^\top (\bX\bbeta+\widetildebX\bbeta-\by);\\
	\nabla_{\widetildebX} L(\bbeta, \widetildebX,\blambda) &= \widetildebX+\blambda\bbeta^\top = \bzero \quad\implies\quad \widetildebX=-\blambda\bbeta^\top,
\end{aligned}
$$
where $\blambda\in\real^n$  is a vector of Lagrange multipliers.
Substituting the value of the vanishing gradient into $(\bX+\widetildebX) \bbeta = \by$ yields $\blambda = \frac{\bX\bbeta-\by}{\bbeta^\top\bbeta}$ and $\widetildebX=-\frac{(\bX\bbeta-\by)\bbeta^\top}{\bbeta^\top\bbeta} $.
Therefore, using the invariance of cyclic permutation of factors in trace,  the objective function becomes 
\begin{equation}\label{equation:dls}
\mathop{\argmin}_{\bbeta}
\frac{(\bX\bbeta-\by)^\top (\bX\bbeta-\by)}{\bbeta^\top\bbeta} .
\end{equation}

\paragrapharrow{Total least squares.} Similar to the data least squares approach, the \textit{total least squares (TLS)} method considers errors in both the predictor variables and the response variables. The TLS problem can be formulated as:
\begin{equation}\label{equation:raw_tls}
\widehatbbeta_{\text{TLS}}, \widetildebX_{\text{TLS}}, \widetildeby_{\text{TLS}} 
= \mathop{\argmin}_{\bbeta, \widetildebX, \widetildeby} \normfbig{[\widetildebX, \widetildeby]}^2, 
\quad \text{s.t.}\quad (\by+\widetildeby)\in\cspace(\bX+\widetildebX), 
\end{equation}
where $\widetilde{\bX}$ and $\widetilde{\by}$ are perturbations in the predictor variables and the response variable, respectively.
Let $\bC\triangleq[\bX,\by]\in\real^{n\times (p+1)}$,  $\bD\triangleq[\widetildebX, \widetildeby]\in\real^{n\times (p+1)}$, and $\bgamma\triangleq
\begin{bmatrixfoot}
\bbeta\\
-1
\end{bmatrixfoot}$, the problem can be equivalently stated as
\begin{equation}
\widehatbbeta_{\text{TLS}}, \widetildebX_{\text{TLS}}, \widetildeby_{\text{TLS}}  = \mathop{\argmin}_{\bgamma, \bD} \normf{\bD}^2, 
\quad \text{s.t.}\quad \bD\bgamma = -\bC\bgamma, 
\end{equation}

\paragrapharrow{Scaled total least squares.}
\citet{paige2002unifying} presents a unified framework that includes OLS, DLS, and TLS as special cases within the following optimization problem:
\begin{equation}
\mathop{\argmin}_{\bbeta, \widetildebX, \widetildeby} \normfbig{[\widetildebX, \gamma\widetildeby]}^2, 
\quad \text{s.t.}\quad (\by+\widetildeby)\in\cspace(\bX+\widetildebX), 
\end{equation}
where $\gamma$ is a given positive scaling parameter. 
For small values of $\gamma$, perturbations in $\by$ will be favored. 
In the limit as $\gamma\rightarrow 0$, the solution equals the ordinary LS solution. 
Conversely, in the limit when $\gamma\rightarrow \infty$, it reduces to the data least squares.

\subsection{Minimum Perturbation in Total Least Squares}
Writing the constraint $(\bX + \widetildebX)\bbeta = \by +\widetildeby$ as
\begin{equation}\label{equation:mnini_per_1}
[\bX + \widetildebX, \by +\widetildeby]
\begin{bmatrix}
\bbeta \\
-1
\end{bmatrix}
= \bzero
\end{equation}
shows that the matrix $[\bX + \widetildebX,  \by +\widetildeby]$ is rank-deficient and that $[\bbeta^\top,  -1]^\top$ is a right singular vector corresponding to a zero singular value: $[\bbeta^\top,  -1]^\top\in\nspace([\bX + \widetildebX,  \by +\widetildeby])$ (Theorem~\ref{theorem:svd-four-orthonormal-Basis}). The TLS problem can be analyzed using the SVD  of the augmented matrix:
\begin{equation}\label{equation:tls_aug_svd}
\bC \triangleq [\bX , \by] = \bU \bSigma \bV^\top = \sum_{i=1}^{p+1} \sigma_i \bu_i \bv_i^\top.
\end{equation}
Suppose $\bX\in\real^{n\times p}$ has full column rank.
Note that as long as the observation vector $\by$ does not lie entirely in the subspace spanned by the columns of matrix $\bX$, the augmented matrix $[\bX, \by]$ has rank $p+1$. That is, the $p+1$ columns of $[\bX, \by]$ are linearly independent. The $p+1$, $n$-dimensional columns of the matrix $[\bX, \by]$ span the $p$ dimensional space spanned by $\bX$ and a component that is normal to the subspace spanned by $\bX$.

For the solution $\bbeta$ to be \textbf{unique}, the matrix $[\bX+ \widetilde{\bX}, \; \by + \widetilde{\by}]$ must have exactly $p$ linearly independent columns. Since this matrix has $p+1$ columns in all, it must be rank-deficient by 1. Therefore, the goal of solving the minimization problem \eqref{equation:raw_tls} can be restated as the goal of finding the ``smallest'' matrix $[\widetildebX, \widetildeby]$ that changes $[\bX, \by]$ with rank $p+1$ to $[\bX, \by] + [\widetildebX, \widetildeby]$ with rank $p$. The Eckart-Young-Mirsky theorem (Theorem~\ref{theorem:young-theorem_frob}) provides the means to do so, by defining $[[\bX, \by] + [\widetildebX, \widetildeby]]$ as the ``best'' rank-$p$ approximation to $[\bX, \by]$. Dropping the smallest singular value of $[\bX, \by]$ eliminates the least amount of information from the data and ensures a unique solution (assuming $\sigma_{p+1}$ is not very close to $\sigma_p$):
%Assume that $\sigma_{p+1} > 0$. Then, by the Eckart-Young-Mirsky theorem (Theorem~\ref{theorem:young-theorem_frob}) the \textbf{unique} perturbation $[\widetildebX, \widetildeby ]$ of minimum Frobenius norm that makes $(\bX + \widetildebX)\bbeta = \by +\widetildeby$ consistent is the rank-one perturbation
\begin{equation}\label{equation:tls_rkone_pert}
\widetildebC \triangleq[\widetildebX ,\widetildeby] = -\sigma_{p+1} \bu_{p+1} \bv_{p+1}^\top,
\end{equation}
and $\min_{\widetildebX,\widetildeby} \normfbig{[\widetildebX , \widetildeby]} = \sigma_{p+1}$. Multiplying \eqref{equation:tls_rkone_pert} from the right with $\bv_{p+1}$ and using \eqref{equation:tls_aug_svd} gives
\begin{equation}\label{equation:tls_ortho_cd}
[\widetildebX, \widetildeby] \bv_{p+1} = -\sigma_{p+1} \bu_{p+1} = - [\bX , \by]\bv_{p+1}
\quad\implies\quad 
[\bX + \widetildebX ,  \by +\widetildeby]\bv_{p+1} = \bzero.
\end{equation}

Then the TLS can be categorized into two forms:
\begin{itemize}
\item \textit{Generic TLS.} $v_{p+1,p+1} \neq 0$, i.e., the $(p+1)$-th component of $\bv_{p+1} $ is nonzero. 
Then  \eqref{equation:mnini_per_1} and \eqref{equation:tls_ortho_cd} show that the TLS solution is obtained by scaling $\bv_{p+1}$ so that its last component is $-1$:
\begin{equation}
\begin{bmatrix} \widehatbbeta_{\text{TLS}} \\ -1 \end{bmatrix} = -\frac{1}{\gamma} \bv_{p+1}, \quad \text{with }\gamma \triangleq  \be_{p+1}^\top \bv_{p+1}.
\end{equation}
Finally, the ``curve-fit'' or prediction is provided by $\widehatby_{\text{TLS}} = (\bX + \widetildebX) \widehatbbeta_{\text{TLS}}$, which requires the parameters, $\widehatbbeta_{\text{TLS}}$, as well as the perturbation in the predictor variables, $\widetildebX$.
Parameter values obtained from TLS cannot be compared directly to those from OLS because the TLS solution is in terms of a different basis (here $\bX + \widetildebX$ instead of $\bX$). This last point complicates the application of TLS to curve-fitting problems in which a parameterized functional form $\widehat{y}(\bx; \widehatbbeta)$ is ultimately desired. Computing $\widehat{\by}$ directly using $\bX\widehatbbeta_{\text{TLS}}$ can give bizarre results.
\item \textit{Nongeneric TLS.} $v_{p+1,p+1} = 0$. The TLS problem fails to have a solution. Nongeneric TLS problems can be treated by adding constraints on the solution \citep{van1989analysis, van1991total, markovsky2007overview}.
\end{itemize}

\paragrapharrow{Normal equation for TLS.}
Alternatively, from the relationship between the SVD of $\bC = [\bX , \by]$ and the spectral decomposition of the symmetric matrix $\bC^\top \bC$~\footnote{If $\bC=\bU\bSigma\bV$ is the SVD of $\bC$, then $\bV\bSigma^2\bV^\top$ is the spectral decomposition of $\bC^\top\bC$; see the proof of SVD in Section~\ref{section:SVD}.}, it follows that the TLS solution $\bbeta$ can be characterized by the following normal equation:
\begin{equation}\label{equation:tls_aug_xtxxty}
\begin{bmatrix}
\bX^\top \bX & \bX^\top \by \\
\by^\top \bX & \by^\top \by
\end{bmatrix}
\bv = \sigma_{p+1}^2 \bv, \quad \bv = \begin{bmatrix}
\bbeta \\
-1
\end{bmatrix},
\end{equation}
where $\sigma_{p+1}^2$ is the smallest eigenvalue of the matrix $\bC^\top \bC$, and $\bv$ is a corresponding eigenvector. From \eqref{equation:tls_aug_xtxxty} it follows that
\begin{equation}\label{equation:tls_normequ}
(\bX^\top \bX - \sigma_{p+1}^2 \bI_p) \bbeta = \bX^\top \by
\qquad\text{and}\qquad 
\by^\top (\by - \bX\bbeta) = \sigma_{p+1}^2.
\end{equation}
In the first equation of \eqref{equation:tls_normequ}, a positive multiple of the unit matrix is subtracted from the matrix of normal equation $\bX^\top \bX \bbeta = \bX^\top \by$. This shows that TLS can be considered as a procedure for deregressionalizing the LS problem. (Compare with Tikhonov regularization, where a multiple of the unit matrix is added to improve the conditioning; see \eqref{equation:tik_normal_equ}). From a statistical point of view, TLS can be interpreted as removing bias by subtracting the error covariance matrix estimated by $\sigma_{p+1}^2 \bI$ from the data covariance matrix $\bX^\top \bX$.

Let further $\widehat{\sigma}_i$, $i = 1,2, \ldots, p$, be the singular values of $\bX$. The \textit{interlacing property of singular values} shows~\footnote{See, for example, \citet{golub2013matrix, lu2021numerical}.}:
$$
\sigma_1 \geq \widehat{\sigma}_1 \geq \ldots \geq \sigma_p \geq \widehat{\sigma}_p \geq \sigma_{p+1}.
$$
The condition $\widehat{\sigma}_p > \sigma_{p+1}$ ensures that $\bX^\top \bX - \sigma_{p+1}^2 \bI$ is symmetric positive definite  by the eigenvalue characterization theorem (Theorem~\ref{theorem:eigen_charac}) and that the TLS problem has a \textbf{unique} solution.

\index{ALS}
\index{Alternating least squares}
\index{Netflix}
\section{Alternating Least Squares (ALS)}\label{section:als-netflix}
The explosion of data from advancements in sensor technology and computer hardware poses new challenges for data analysis. 
The substantial  volume of data often contains noise and other distortions, requiring pre-processing for the application of deductive science. 
For instance, signals received by antenna arrays often are contaminated by noise and other degradations. 
Effectively analyzing such data requires reconstruction or representation in a manner that minimizes inaccuracies while maintaining certain feasibility conditions.

Moreover, data collected from complex systems often arises from multiple interrelated variables acting together. 
When these variables lack clear definitions, the information contained in the original data may be overlapping and unclear. By creating a reduced system model, we can achieve a level of accuracy that is close to the original system. 
The standard approach involves removing noise, reducing the model, and reconstructing feasibility by replacing the original data with a lower-dimensional representation obtained through subspace approximation.
Consequently, low-rank approximations or low-rank matrix decompositions play a important role in 
a wide range of applications.

Low-rank matrix decomposition stands out as a potent technique in machine learning and data mining for representing a given matrix as the product of two or more matrices with lower dimensions.
This method captures the essential structure of a matrix while disregarding noise and redundancies.  
Common techniques for low-rank matrix decomposition  include singular value decomposition (SVD), principal component analysis (PCA),  multiplicative update nonnegative matrix factorization (NMF), and the alternating least squares (ALS) approach introduced in this section.

For example, in the Netflix Prize competition \citep{bennett2007netflix}, the objective is to predict the users' ratings for different movies based on their existing ratings  for other movies.
We use indices $n= 1, 2,\ldots,N$ for $N$ movies and $p = 1, 2,\ldots,P$ for $P$ users. 
The rating of the $p$-th user for the $n$-th movie is denoted by $x_{np}$. 
Let $\bX$ be an $N \times P$ \footnote{For the purpose of this section, we temporarily assume that the matrix $\bX$ has dimensions $N \times P$; otherwise, it will be treated as $n \times p$.} rating matrix with columns $\bx_p \in \real^N$ containing ratings provided by the $p$-th user. Note that numerous  ratings $\{x_{np}\}$ are missing, and our objective is to accurately predict these absent  ratings.

We formally consider algorithms for solving the following problem: Approximating the matrix $\bX$ through factorization into an  $N\times K$ matrix $\bW$ and a $K \times  P$ matrix $\bZ$. 
Typically, $K$ is chosen to be smaller than both $N$ and $P$, ensuring reduced dimensions for $\bW$ and $\bZ$  compared to the original  matrix $\bX$. 
This dimensional reduction yields a compressed version of the original data matrix. 
Deciding the appropriate value for $K$ is crucial in practice, and its selection is often problem-dependent.

The factorization holds significance; let $\bX=[\bx_1, \bx_2, \ldots, \bx_P]$ and $\bZ=[\bz_1, \bz_2, \ldots, \bz_P]$ be the column partitions of $\bX$ and $\bZ$, respectively. 
Then,  $\bx_p = \bW\bz_p$, implying that each column $\bx_p$ is approximated by a linear combination of the columns of $\bW$, weighted by the components in $\bz_p$. 
Thus, the columns of $\bW$ can be viewed as containing the column basis of $\bX$.

To achieve the approximation  $\bX\approx\bW\bZ$, a suitable loss function must be established for measuring the distance between $\bX$ and $\bW\bZ$. 
In this context, we opt for the Frobenius norm (Definition~\ref{definition:frobernius-in-svd}) between two matrices, which vanishes to zero if $\bX=\bW\bZ$, and the advantage will be evident shortly.

To simplify the problem, let's first assume  the absence of missing ratings. 
We project data vectors $\bx_p\in\real^N$ into a lower dimension $\bz_p \in \real^K$  with $K<\min\{N, P\}$
in a manner that minimizes the \textit{reconstruction error}, as measured by the Frobenius norm (assuming $K$ is known):
\begin{equation}\label{equation:als-per-example-loss}
\mathop{\min}_{\bW,\bZ}  \sum_{p=1}^P \sum_{n=1}^{N} \left(x_{np} - \bw_n^\top\bz_p\right)^2,
\end{equation}
where $\bW=[\bw_1^\top; \bw_2^\top; \ldots; \bw_N^\top]\in \real^{N\times K}$ and $\bZ=[\bz_1, \bz_2, \ldots, \bz_P] \in \real^{K\times P}$ contain $\bw_n$'s and $\bz_p$'s as \textbf{rows and columns}, respectively. The loss formulation in \eqref{equation:als-per-example-loss} is referred to as the \textit{per-example loss}. 
It can be equivalently expressed as
$$
L(\bW,\bZ) = \sum_{p=1}^P \sum_{n=1}^{N} \left(x_{np} - \bw_n^\top\bz_p\right)^2 
= \normf{\bW\bZ-\bX}^2.
$$ 
Furthermore, the loss function $L(\bW,\bZ)=\sum_{p=1}^P \sum_{n=1}^{N} \left(x_{np} - \bw_n^\top\bz_p\right)$ is convex concerning $\bZ$ given $\bW$, and vice versa. Therefore, we can first minimize it with respect to $\bZ$
while keeping $\bW$ fixed, and subsequently minimize it with respect to $\bW$ with $\bZ$ fixed.
This results in  two optimization problems, denoted by ALS1 and ALS2:
$$
\left\{
\begin{aligned}
\bZ &\leftarrow \mathop{\arg \min}_{\bZ} L(\bW,\bZ);    \qquad \text{(ALS1)} \\ 
\bW &\leftarrow \mathop{\arg \min}_{\bW} L(\bW,\bZ). \qquad \text{(ALS2)}
\end{aligned}
\right.
$$
This is referred to as the \textit{coordinate descent algorithm}, wherein  we alternate between optimizing the least squares concerning $\bW$ and $\bZ$. 
Therefore, it is also called the \textit{alternating least squares (ALS)} algorithm \citep{comon2009tensor, takacs2012alternating, giampouras2018alternating}. 
The convergence is guaranteed if the loss function $L(\bW,\bZ)$ decreases at each iteration.

\index{Coordinate descent algorithm}
\index{Convexity}
\index{Global minimum}
\index{ALS}
\begin{remark}[Convexity and global minimum]
While the loss function defined by Frobenius norm $\normf{\bW\bZ-\bX}^2$ is convex either with respect to  $\bW$ when $\bZ$ is fixed or vice versa, it lacks joint convexity in both variables simultaneously. Consequently, identifying the global minimum is infeasible. 
Nevertheless, the convergence is guaranteed to reach local minima.
\end{remark}

\subsection*{Given $\bW$, Optimizing $\bZ$}

Let's now explore the problem of $\bZ \leftarrow \mathop{\arg \min}_{\bZ} L(\bW,\bZ)$. 
%When a unique minimum of the loss function $L(\bW,\bZ)$ exists concerning $\bZ$, we denote it as the \textit{least squares} minimizer of $\mathop{\arg \min}_{\bZ} L(\bW,\bZ)$. 
With $\bW$ fixed, we can represent $L(\bW,\bZ)$  as $L(\bZ\mid \bW)$ (or more concisely, as $L(\bZ)$) to emphasize  the variable of $\bZ$:
$$
\begin{aligned}
L(\bZ\mid \bW) &= \normf{\bW\bZ-\bX}^2= \normtwobig{\bW[\bz_1,\bz_2,\ldots, \bz_P]-[\bx_1,\bx_2,\ldots,\bx_P]}^2=
\normtwo{\begin{bmatrixfoot}
		\bW\bz_1 - \bx_1 \\
		\bW\bz_2 - \bx_2\\
		\vdots \\
		\bW\bz_P - \bx_P
\end{bmatrixfoot}}^2. 
\end{aligned}
$$
Now, if we define 
$$
\widetildebW 
\triangleq 
\begin{bmatrixfoot}
\bW & \bzero & \ldots & \bzero\\
\bzero & \bW & \ldots & \bzero\\
\vdots & \vdots & \ddots & \vdots \\
\bzero & \bzero & \ldots & \bW
\end{bmatrixfoot}
\in \real^{NP\times KP}, 
\gap 
\widetildebz
\triangleq
\begin{bmatrixfoot}
\bz_1 \\ \bz_2 \\ \vdots \\ \bz_P
\end{bmatrixfoot}
\in \real^{KP},
\gap 
\widetildeba
\triangleq
\begin{bmatrixfoot}
\bx_1 \\ \bx_2 \\ \vdots \\ \bx_P
\end{bmatrixfoot}
\in \real^{NP},
$$
then the (ALS1) problem an be equivalently transformed into the ordinary least squares problem, aiming to minimize $\normtwobig{\widetildebW \widetildebz - \widetildeba}^2$ concerning $\widetildebz$. The solution is then given by 
$$
\widetildebz = (\widetildebW^\top\widetildebW)^{-1} \widetildebW^\top\widetildeba.
$$
However, it is not advisable  to employ this approach for obtaining the result, as computing the inverse of  $\widetildebW^\top\widetildebW$ requires $2(KP)^3$ flops \citep{lu2021numerical}.
Instead, a more direct method to solve the (ALS1) problem is to determine the gradient of $L(\bZ\mid \bW)$ concerning $\bZ$ (assuming all the partial derivatives of this function exist):
\begin{equation}\label{equation:givenw-update-z-allgd}
\begin{aligned}
\nabla L(\bZ\mid \bW) &= 
\frac{\partial \,\,\trace\left((\bW\bZ-\bX)(\bW\bZ-\bX)^\top\right)}{\partial \bZ}
%&=\frac{\partial \,\,\trace\left((\bW\bZ-\bX)(\bW\bZ-\bX)^\top\right)}{\partial (\bW\bZ-\bX)}
%\frac{\partial (\bW\bZ-\bX)}{\partial \bZ}\\
= 2  \bW^\top(\bW\bZ-\bX) \in \real^{K\times P},
\end{aligned}
\end{equation}
%where the first equality arises from the definition of the Frobenius norm (Definition~\ref{definition:frobernius-in-svd}) such that $\norm{\bX} = \sqrt{\sum_{n=1,p=1}^{N,P} (x_{np})^2}=\sqrt{\trace(\bX\bX^\top)}$, and equality ($\star$)  follows from the fact that $\frac{\partial \trace(\bX\bX^\top)}{\partial \bX} = 2\bX$. 
When the loss function is a differentiable function of $\bZ$, we can determine the least squares solution using differential calculus. And a minimum of the function 
$L(\bZ\mid \bW)$ must be a root of the equation (Proposition~\ref{proposition:fermat_fist_opt}):
$$
\nabla L(\bZ\mid \bW)  = \bzero.
$$
By solving the equation above, we derive the ``candidate" update for $\bZ$, which corresponds to the minimizer of $L(\bZ\mid \bW)$:
\begin{equation}\label{equation:als-z-update}
{\bZ = (\bW^\top\bW)^{-1} \bW^\top \bX  \leftarrow \mathop{\arg \min}_{\bZ} L(\bZ\mid \bW).}
\end{equation}
This requires $2K^3$ flops to compute the inverse of $\bW^\top\bW$, a notable improvement compared to $2(KP)^3$ flops to get the inverse of $\widetildebW^\top\widetildebW$ \citep{lu2021numerical}.
Before we declare a root of the  equation above is actually a minimizer rather than a maximizer (that's why we call the update a ``candidate" update), we need to verify the function is convex. 
In the case where  the function is twice differentiable, this confirmation   can be equivalently achieved by verifying (see Problem~\ref{problem:pos_hessian}): 
$$
\nabla^2 L(\bZ\mid \bW) \succ \bzero.~
\footnote{In short, a twice continuously differentiable function $f$ over an open convex set $\sS$ is called \textit{convex} if and only if $\nabla^2f(\bbeta)\geq \bzero $ for any $\bbeta\in \sS$ (sufficient and necessary for convex); and called \textit{strictly convex} if $\nabla^2f(\bbeta)> \bzero$ for any $\bbeta\in \sS$ (only sufficient for strictly convex, e.g., $f(\beta)=\beta^6$ is strictly convex, but $f^{\prime\prime}(\beta)=30\beta^4$ is equal to zero at $\beta=0$.). 
And when the convex function $f$ is a continuously differentiable function over a convex set $\sS$, the stationary point $\nabla f(\bbeta^\star)=\bzero$ of $\bbeta^\star\in\sS$ is  a \textit{global minimizer} of $f$ over $\sS$.
In our context, when given $\bW$ and updating $\bZ$, the function is defined over the entire space $\real^{K\times P}$.
}
$$
That is, the Hessian matrix is positive definite. To see this, we explicitly express the Hessian matrix as
\begin{equation}\label{equation:als-z-update_hessian}
\nabla^2 L(\bZ\mid \bW)= 2\widetildebW^\top\widetildebW \in \real^{KP\times KP},
\end{equation}
which has full rank if $\bW\in \real^{N\times K}$ has full rank  and $K<N$ (Lemma~\ref{lemma:rank-of-ata-x}).

\begin{remark}[Positive definite Hessian if $\bW$ has full rank]
We assert that if $\bW\in\real^{N\times K}$ has full rank $K$ with $K<N$, then $\nabla^2 L(\bZ\mid \bW)$ is positive definite. This assertion is supported by verifying that when $\bW$ has full rank, the equation $\bW\bbeta=\bzero$ only holds true when $\bbeta=\bzero$, since the null space of $\bW$ is of dimension 0. Therefore, 
$$
\bbeta^\top (2\bW^\top\bW)\bbeta >0, \qquad \text{for any nonzero vector $\bbeta\in \real^K$}.
$$ 
\end{remark}
The challenge now is to confirm \textcolor{black}{\textbf{whether $\bW$ possesses full rank, ensuring the positive definiteness of the Hessian of $L(\bZ\mid \bW)$}}; otherwise, we cannot claim the update of $\bZ$ in Equation~\eqref{equation:als-z-update} reduces the loss (due to convexity), thereby enhancing the matrix decomposition's approximation of the original matrix $\bX$ through $\bW\bZ$ in each iteration. 
We will shortly come back to the positive definiteness of the Hessian matrix in the sequel, relying on the following lemma.
\begin{lemma}[Rank of $\bZ$ after updating]\label{lemma:als-update-z-rank}
Suppose $\bX\in \real^{N\times P}$ has full rank with \textcolor{mylightbluetext}{$N\leq P$} and $\bW\in \real^{N\times K}$ has full rank with $K<N$ (i.e., $K<N\leq P$).
Then the update of $\bZ=(\bW^\top\bW)^{-1} \bW^\top \bX \in \real^{K\times P}$ in Equation~\eqref{equation:als-z-update} has full rank.
\end{lemma}
\begin{proof}[of Lemma~\ref{lemma:als-update-z-rank}]
Since $\bW^\top\bW\in \real^{K\times K}$ has full rank if $\bW$ has full rank (Lemma~\ref{lemma:rank-of-ata-x}), it follows that $(\bW^\top\bW)^{-1} $ has full rank. 

Suppose $\bW^\top\bbeta=\bzero$, it implies that $(\bW^\top\bW)^{-1} \bW^\top\bbeta=\bzero$. Thus, the following two null spaces satisfy:
$$
\nspace(\bW^\top) \subseteq \nspace\left((\bW^\top\bW)^{-1} \bW^\top\right).
$$
Furthermore, suppose $\bbeta$ is in the null space of $(\bW^\top\bW)^{-1} \bW^\top$ such that $(\bW^\top\bW)^{-1} \bW^\top\bbeta=\bzero$. And since $(\bW^\top\bW)^{-1} $ is invertible, this implies $ \bW^\top\bbeta=(\bW^\top\bW)\bzero=\bzero$, and 
$$
\nspace\left((\bW^\top\bW)^{-1} \bW^\top\right)\subseteq \nspace(\bW^\top).
$$
Combining the two results yields that 
\begin{equation}\label{equation:als-z-sandiwch1}
\nspace(\bW^\top) = \nspace\left((\bW^\top\bW)^{-1} \bW^\top\right).
\end{equation}
Hence, $(\bW^\top\bW)^{-1} \bW^\top$ has full rank $K$. Let $\bT\triangleq (\bW^\top\bW)^{-1} \bW^\top\in \real^{K\times N}$, and suppose $\bT^\top\bbeta=\bzero$. This implies $\bX^\top\bT^\top\bbeta=\bzero$, and 
$$
\nspace(\bT^\top) \subseteq \nspace(\bX^\top\bT^\top).
$$
Similarly, suppose $\bX^\top(\bT^\top\bbeta)=\bzero$. Since $\bX$ has full rank with the dimension of the null space being 0: $\dim\left(\nspace(\bX^\top)\right)=0$, $(\bT^\top\bbeta)$ must be zero. The claim follows  since $\bX$ has full rank $N$ with the row space of $\bX^\top$ being equal to the column space of $\bX$, where $\dim\left(\cspace(\bX)\right)=N$ and the $\dim\left(\nspace(\bX^\top)\right) = N-\dim\left(\cspace(\bX)\right)=0$. Therefore, $\bbeta$ is in the null space of $\bT^\top$ if $\bbeta$ is in the null space of $\bX^\top\bT^\top$:
$$
\nspace(\bX^\top\bT^\top)\subseteq \nspace(\bT^\top).
$$
By ``sandwiching" again, 
\begin{equation}\label{equation:als-z-sandiwch2}
\nspace(\bT^\top) = \nspace(\bX^\top\bT^\top).
\end{equation}
Since $\bT^\top$ has full rank $K<N\leq P$, it follows that $\dim\left(\nspace(\bT^\top) \right) = \dim\left(\nspace(\bX^\top\bT^\top)\right)=0$.
Therefore,
$\bZ^\top=\bX^\top\bT^\top$ has full rank $K$.
We complete the proof.
\end{proof}

\subsection*{Given $\bZ$, Optimizing $\bW$}

Similarly, given $\bZ$ fixed, express $L(\bW,\bZ)$ as $L(\bW\mid \bZ)$ (or more concisely, as $L(\bW)$) to emphasize the variable of $\bW$:
$
\begin{aligned}
L(\bW\mid \bZ) &= \normf{\bW\bZ-\bX}^2.
\end{aligned}
$
A direct approach to solve the optimization of (ALS2) involves finding the gradient of $L(\bW\mid \bZ)$ with respect to $\bW$:
$$
\begin{aligned}
\nabla L(\bW\mid \bZ) &= 
\frac{\partial \,\,\trace\left((\bW\bZ-\bX)(\bW\bZ-\bX)^\top\right)}{\partial \bW}
%&=\frac{\partial \,\,\trace\left((\bW\bZ-\bX)(\bW\bZ-\bX)^\top\right)}{\partial (\bW\bZ-\bX)}
%\frac{\partial (\bW\bZ-\bX)}{\partial \bW}\\
=  2(\bW\bZ-\bX)\bZ^\top \in \real^{N\times K}.
\end{aligned}
$$
Similarly, the ``candidate" update for  $\bW$ can be obtained by locating the root of the gradient $\nabla L(\bW\mid \bZ)$:
\begin{equation}\label{equation:als-w-update}
{\bW^\top = (\bZ\bZ^\top)^{-1}\bZ\bX^\top  \leftarrow \mathop{\arg\min}_{\bW} L(\bW\mid \bZ).}
\end{equation}
Again, it is important to highlight that the provided update is merely a ``candidate" update. 
Further verification is required to determine whether the Hessian is positive definite or not. 
The Hessian matrix is expressed as follows:
\begin{equation}\label{equation:als-w-update_hessian}
\begin{aligned}
\nabla^2 L(\bW\mid \bZ) =2\widetildebZ\widetildebZ^\top \in \real^{KN\times KN}.
\end{aligned}
\end{equation}
Therefore, by analogous analysis, if $\bZ$ has full rank with $K<P$, the Hessian matrix is positive definite.

\begin{lemma}[Rank of $\bW$ after updating]\label{lemma:als-update-w-rank}
Suppose $\bX\in \real^{N\times P}$ has full rank with \textcolor{mylightbluetext}{$N\geq P$} and $\bZ\in \real^{K\times P}$ has full rank with $K<P$ (i.e., $K<P\leq N$). 
Then the update of $\bW^\top = (\bZ\bZ^\top)^{-1}\bZ\bX^\top$ in Equation~\eqref{equation:als-w-update} has full rank.
\end{lemma}
The proof of Lemma~\ref{lemma:als-update-w-rank} is similar to that of  Lemma~\ref{lemma:als-update-z-rank}, and we shall not repeat the details.

\paragrapharrow{Key observation.}
Combining the observations in Lemma~\ref{lemma:als-update-z-rank} and Lemma~\ref{lemma:als-update-w-rank}, as long as we \textcolor{mylightbluetext}{initialize $\bZ$ and $\bW$ to have full rank}, the updates in Equations~\eqref{equation:als-z-update} and \eqref{equation:als-w-update} are reasonable \textbf{since the Hessians in Equations~\eqref{equation:als-z-update_hessian} and \eqref{equation:als-w-update_hessian} are positive definite}. 
{
Note that we need an additional condition to satisfy  
both Lemma~\ref{lemma:als-update-z-rank} 
and Lemma~\ref{lemma:als-update-w-rank}: $N=P$, i.e., there must be an equal number of movies and   users.
We will relax this condition  through regularization.
} 
We summarize the process in Algorithm~\ref{alg:als}.

\begin{algorithm}[htp] 
\caption{Alternating Least Squares}
\label{alg:als}
\begin{algorithmic}[1] 
\Require Matrix $\bX\in \real^{N\times P}$ \textcolor{mylightbluetext}{with $N= P$};
\State Initialize $\bW\in \real^{N\times K}$, $\bZ\in \real^{K\times P}$ \textcolor{mylightbluetext}{with full rank and $K<N= P$}; 
\State Choose a stop criterion on the approximation error $\delta$;
\State Choose the maximal number of iterations $C$;
\State $iter=0$; \Comment{Count for the number of iterations}
\While{$\normf{\bX-\bW\bZ}>\delta $ and $iter<C$} 
\State $iter=iter+1$;
\State $\bZ = (\bW^\top\bW)^{-1} \bW^\top \bX  \leftarrow \mathop{\arg \min}_{\bZ} L(\bZ\mid \bW)$;
\State $\bW^\top = (\bZ\bZ^\top)^{-1}\bZ\bX^\top  \leftarrow \mathop{\arg\min}_{\bW} L(\bW\mid \bZ)$;
\EndWhile
\State Output $\bW,\bZ$;
\end{algorithmic} 
\end{algorithm}

\index{Regularization}
\subsection*{Regularization: Extension to General Matrices}\label{section:regularization-extention-general}

\textit{Tikhonov regularization} or simply \textit{regularization} is a machine learning technique employed to prevent overfitting and enhance model generalization; see Section~\ref{section:totalles_otheriss}. Overfitting occurs when a model becomes excessively complex, closely fitting the training data but performing poorly on unseen data.
To address this issue, regularization introduces a constraint or penalty term into the loss function used for model optimization. This discourages the development of overly complex models, striking a balance between model simplicity and effective training data fitting.
Common types of regularization include $\ell_1$ regularization, $\ell_2$ regularization, and elastic net regularization (a combination of $\ell_1$ and $\ell_2$ regularization). 
Regularization finds extensive application in machine learning algorithms such as linear regression, logistic regression, and neural networks.

In the context of the alternating least squares problem, we can introduce a $\ell_2$ regularization term  to minimize the following loss:
\begin{equation}\label{equation:als-regularion-full-matrix}
L(\bW,\bZ)  =\normf{\bW\bZ-\bX}^2 +\lambda_w \normf{\bW}^2 + \lambda_z \normf{\bZ}^2, \qquad \lambda_w>0, \lambda_z>0,
\end{equation}
where the gradient with respect to $\bZ$ and $\bW$ are given respectively by 
\begin{equation}\label{equation:als-regulari-gradien}
\left\{
\begin{aligned}
\nabla L(\bZ\mid \bW) &= 2\bW^\top(\bW\bZ-\bX) + 2\lambda_z\bZ \in \real^{K\times P};\\
\nabla L(\bW\mid \bZ)  &= 2(\bW\bZ-\bX)\bZ^\top + 2\lambda_w\bW \in \real^{N\times K}.
\end{aligned}
\right.
\end{equation}
The Hessian matrices become:
\begin{equation}\label{equation:als_hess_reg}
\left\{
\begin{aligned}
\nabla^2 L(\bZ\mid \bW) &= 2\widetildebW^\top\widetildebW+ 2\lambda_z\bI \in \real^{KP\times KP};\\
\nabla^2 L(\bW\mid \bZ)  &= 2\widetildebZ\widetildebZ^\top + 2\lambda_w\bI \in \real^{KN\times KN}, \\
\end{aligned}
\right.
\end{equation}
which are positive definite due to the perturbation introduced by the regularization. 

%To see this, we have 
%$$
%\left\{
%\begin{aligned}
%\bbeta^\top (2\widetildebW^\top\widetildebW +2\lambda_z\bI)\bbeta 
%&= \underbrace{2\widetildebx^\top\widetildebW^\top\widetildebW\bbeta}_{\geq 0} + 2\lambda_z \normtwo{\bbeta}^2>0, \gap \text{for nonzero $\bbeta$};\\
%\bbeta^\top (2\widetildebZ\widetildebZ^\top +2\lambda_w\bI)\bbeta 
%&= \underbrace{2\bbeta^\top\widetildebZ\widetildebZ^\top\bbeta}_{\geq 0} + 2\lambda_w \normtwo{\bbeta}^2>0,\gap \text{for nonzero $\bbeta$}.
%\end{aligned}
%\right.
%$$
\textbf{The regularization ensues that the Hessian matrices become positive definite, even if $\bW$ and $\bZ$ are rank-deficient}. 
Consequently, matrix decomposition can be extended to any matrix, irrespective of  whether $N>P$ or $N<P$. In rare cases, $K$ can be chosen as $K>\max\{N, P\}$ to obtain a \textit{high-rank approximation} of $\bX$. 
However, in most scenarios, we want to find the \textit{low-rank approximation} of $\bX$ with $K<\min\{N, P\}$. For instance,  ALS can be utilized to find  low-rank neural networks or transformer structures, reducing the memory usage of neural networks while enhancing performance \citep{lu2021numerical, lu2025large}.
Therefore, the minimizers can be determined by identifying the roots of the gradient:
\begin{equation}\label{equation:als-regular-final-all}
\begin{aligned}
\bZ &= (\bW^\top\bW+ \lambda_z\bI)^{-1} \bW^\top \bX 
\qquad\text{and}\qquad 
\bW^\top = (\bZ\bZ^\top+\lambda_w\bI)^{-1}\bZ\bX^\top .
\end{aligned}
\end{equation}
The regularization parameters $\lambda_z, \lambda_w\in \real$ are used to balance the trade-off
between the accuracy of the approximation and the smoothness of the computed solution. The selection of these parameters is typically problem-dependent and can be obtained through \textit{cross-validation}. Again, we summarize the process in Algorithm~\ref{alg:als-regularizer}.

\index{Cross-validation}
\begin{algorithm}[h] 
\caption{Alternating Least Squares with Regularization}
\label{alg:als-regularizer}
\begin{algorithmic}[1] 
\Require Matrix $\bX\in \real^{N\times P}$;
\State Initialize $\bW\in \real^{N\times K}$, $\bZ\in \real^{K\times P}$ \textcolor{mylightbluetext}{randomly without condition on the rank and the relationship between $N, P, K$}; 
\State Choose a stop criterion on the approximation error $\delta$;
\State Choose regularization parameters $\lambda_w, \lambda_z$;
\State Choose the  maximal number of iterations $C$;
\State $iter=0$; \Comment{Count for the number of iterations}
\While{$\normf{\bX-\bW\bZ}>\delta $ and $iter<C$}
\State $iter=iter+1$; 
\State $\bZ = (\bW^\top\bW+ \lambda_z\bI)^{-1} \bW^\top \bX  \leftarrow \mathop{\arg \min}_{\bZ} L(\bZ\mid \bW)$;
\State $\bW^\top = (\bZ\bZ^\top+\lambda_w\bI)^{-1}\bZ\bX^\top  \leftarrow \mathop{\arg\min}_{\bW} L(\bW\mid \bZ)$;
\EndWhile
\State Output $\bW,\bZ$;
\end{algorithmic} 
\end{algorithm}

\index{Missing entries}
\index{Hadamard product}
\index{Netflix recommender}
\subsection{Missing Entries and Rank-One Update}\label{section:alt-columb-by-column}
Since  matrix decomposition via  ALS is extensively used in the context of Netflix recommender data, where a substantial number of entries are missing due to users not having watched certain movies or choosing not to rate them for various reasons.
In this scenario, the low-rank matrix decomposition problem is also known as \textit{matrix completion} that can help recover unobserved entries \citep{jain2017non}.
To address this, we can introduce an additional mask matrix $\bM\in \{0,1\}^{N\times P}$, where $m_{np}\in \{0,1\}$ indicates whether  user $p$ has rated  movie $n$ or not. Therefore, the loss function can be defined as 
$$
L(\bW,\bZ) = \frac{1}{2}\normf{\bM\hadaprod  \bX- \bM\hadaprod (\bW\bZ)}^2,
$$
where $\hadaprod$ represents the \textit{Hadamard product} between matrices. 
%For example, the Hadamard product of a $3 \times 3$ matrix $\bX$ and a $3\times 3$ matrix $\bB$ is
%$$
%\bX\hadaprod \bB = 
%\begin{bmatrix}
%a_{11} & a_{12} & a_{13} \\
%a_{21} & a_{22} & a_{23} \\
%a_{31} & a_{32} & a_{33}
%\end{bmatrix}
%\hadaprod
%\begin{bmatrix}
%b_{11} & b_{12} & b_{13} \\
%b_{21} & b_{22} & b_{23} \\
%b_{31} & b_{32} & b_{33}
%\end{bmatrix}
%=
%\begin{bmatrix}
%a_{11}b_{11} &  a_{12}b_{12} & a_{13}b_{13} \\
%a_{21}b_{21} & a_{22}b_{22} & a_{23}b_{23} \\
%a_{31}b_{31} & a_{32}b_{32} & a_{33}b_{33}
%\end{bmatrix}.
%$$
The above formulation concisely expresses our goal of finding a completion of the ratings matrix that is both of low rank and consistent with observed user ratings.
To find the solution to this problem, we decompose the updates in Equation~\eqref{equation:als-regular-final-all} into:
\begin{equation}\label{equation:als-ori-all-wz}
\left\{
\begin{aligned}
\bz_p &= (\bW^\top\bW+ \lambda_z\bI)^{-1} \bW^\top \bx_p, &\gap& \text{for $p\in \{1,2,\ldots, P\}$}  ;\\
\bw_n &= (\bZ\bZ^\top+\lambda_w\bI)^{-1}\bZ\bb_n,  &\gap& \text{for $n\in \{1,2,\ldots, N\}$} ,
\end{aligned}
\right.
\end{equation}
where $\bZ=[\bz_1, \bz_2, \ldots, \bz_P]$ and $\bX=[\bx_1,\bx_2, \ldots, \bx_P]$ represent the column partitions of $\bZ$ and $\bX$, respectively. Similarly, $\bW^\top=[\bw_1, \bw_2, \ldots, \bw_N]$ and $\bX^\top=[\bb_1,\bb_2, \ldots, \bb_N]$ are the column partitions of $\bW^\top$ and $\bX^\top$, respectively. This decomposition of the updates indicates the updates can be performed in a column-by-column fashion (the rank-one updates).

\paragrapharrow{Given $\bW$.}
Let $\bo_p\in \{0,1\}^N$ represent the movies rated by user $p$, where $o_{pn}=1$ if user $p$ has rated movie $n$, and $o_{pn}=0$ otherwise. Then the $p$-th column of $\bX$ without missing entries can be denoted using the Matlab-style notation as $\bx_p[\bo_p]$. 
And we want to approximate the existing entries of the $p$-th column by $\bx_p[\bo_p] \approx \bW[\bo_p, :]\bz_p$, which is indeed a rank-one least squares problem:
\begin{equation}\label{equation:als-ori-all-wz-modif-z}
\begin{aligned}
\bz_p &= \left(\bW[\bo_p, :]^\top\bW[\bo_p, :]+ \lambda_z\bI\right)^{-1} \bW[\bo_p, :]^\top \bx_p[\bo_p], \quad \text{for $p\in \{1,2,\ldots, P\}$} .
\end{aligned}
\end{equation}
Moreover, the loss function with respect to $\bz_p$ and $\bZ$  can be described, respectively, by
$$
\begin{aligned}
L(\bz_p|\bW) &=\sum_{n\in \bo_p} \left(x_{np} - \bw_n^\top\bz_p\right)^2
\gap \text{and}\gap
L(\bZ|\bW) =\sum_{p=1}^P\ \sum_{n\in \bo_p} \left(x_{np} - \bw_n^\top\bz_p\right)^2.
\end{aligned}
$$

\paragrapharrow{Given $\bZ$.}
Similarly, if $\bp_n \in\{0,1\}^{P}$ denotes the users who have rated  movie $n$, with $p_{np}=1$ if  movie $n$ has been rated by user $p$, and $p_{np}=0$ otherwise. Then the $n$-th row of $\bX$ without missing entries can be denoted by the Matlab-style notation as $\bb_n[\bp_n]$. We want to approximate the existing entries of the $n$-th row by $\bb_n[\bp_n] \approx \bZ[:, \bp_n]^\top\bw_n$, 
\footnote{Note that $\bZ[:, \bp_n]^\top$ is the transpose of $\bZ[:, \bp_n]$, which is equal to $\bZ^\top[\bp_n,:]$, i.e., transposing first and then selecting.}
which  is again a rank-one least squares problem:
\begin{equation}\label{equation:als-ori-all-wz-modif-w}
\begin{aligned}
\bw_n &= (\bZ[:, \bp_n]\bZ[:, \bp_n]^\top+\lambda_w\bI)^{-1}\bZ[:, \bp_n]\bb_n[\bp_n],  \quad \text{for $n\in \{1,2,\ldots, N\}$} .
\end{aligned}
\end{equation}
Similarly, the loss function with respect to $\bw_n$ and $\bW$ can be described, respectively,  by
$$
\begin{aligned}
L(\bw_n|\bZ) &=\sum_{p\in \bp_n} \left(x_{np} - \bw_n^\top\bz_p\right)^2 
\gap \text{and}\gap
L(\bW|\bZ) =\sum_{n=1}^N  \sum_{p\in \bp_n} \left(x_{np} - \bw_n^\top\bz_p\right)^2 .
\end{aligned}
$$
The procedure is once again presented in Algorithm~\ref{alg:als-regularizer-missing-entries}.
Other approaches, such as \textit{singular value projection (SVP)}, also exist to address the matrix completion problem. At a high level, SVP is a type of projected gradient descent (PGD) method that updates iteratively via gradient descent, projecting the updated matrix into a low-rank form through singular value decomposition at each step. However, the alternating least squares approach generally  outperforms SVP in the context of matrix completion, so we will not delve into SVP here. For more details, refer to \citet{jain2017non}.

\begin{algorithm}[h] 
\caption{Alternating Least Squares with Missing Entries and Regularization}
\label{alg:als-regularizer-missing-entries}
\begin{algorithmic}[1] 
\Require Matrix $\bX\in \real^{N\times P}$;
\State Initialize $\bW\in \real^{N\times K}$, $\bZ\in \real^{K\times P}$ \textcolor{mylightbluetext}{randomly without condition on the rank and the relationship between $N, P, K$}; 
\State Choose a stoping criterion on the approximation error $\delta$;
\State Choose regularization parameters $\lambda_w, \lambda_z$;
\State Compute the mask matrix $\bM$ from $\bX$;
\State Choose the maximum number of iterations $C$;
\State $iter=0$; \Comment{Count for the number of iterations}
\While{\textcolor{mylightbluetext}{$\normf{\bM\hadaprod  \bX- \bM\hadaprod (\bW\bZ)}^2>\delta $} and $iter<C$}
\State $iter=iter+1$; 
\For{$p=1,2,\ldots, P$}
\State $\bz_p \leftarrow \left(\bW[\bo_p, :]^\top\bW[\bo_p, :]+ \lambda_z\bI\right)^{-1} \bW[\bo_p, :]^\top \bx_p[\bo_p]$; \Comment{$p$-th column of $\bZ$}
\EndFor

\For{$n=1,2,\ldots, N$}
\State $\bw_n \leftarrow (\bZ[:, \bp_n]\bZ[:, \bp_n]^\top+\lambda_w\bI)^{-1}\bZ[:, \bp_n]\bb_n[\bp_n]$;\Comment{$n$-th column of $\bW^\top$}
\EndFor
\EndWhile
\State Output $\bW^\top=[\bw_1, \bw_2, \ldots, \bw_N],\bZ=[\bz_1, \bz_2, \ldots, \bz_P]$;
\end{algorithmic} 
\end{algorithm}

\index{Hidden features}
\index{Inner product}
\subsection{Vector Inner Product and Hidden Vectors}\label{section:als-vector-product}
We have observed that the ALS algorithm seeks to find lower-dimensional matrices $\bW$ and $\bZ$ such that their product $\bW\bZ$ can approximate $\bX\approx \bW\bZ$ in terms of the  squared loss:
$
\mathop{\min}_{\bW,\bZ}  \sum_{p=1}^P \sum_{n=1}^{N} \left(x_{np} - \bw_n^\top\bz_p\right)^2,
$
that is, each entry $x_{np}$ in $\bX$ can be approximated as the inner product of  two vectors $\bw_n^\top\bz_p$. The geometric definition of the vector inner product is given by 
$$
\bw_n^\top\bz_p = \normtwo{\bw_n}\cdot \normtwo{\bz_p} \cos \theta,
$$
where $\theta$ represents the angle between vectors $\bw_n$ and $\bz_p$. Thus, if the vector norms of $\bw_n$ and $\bz_p$ are determined, the smaller the angle, the larger the inner product.

In the context of Netflix data,   movie ratings range from 0 to 5,
with higher ratings indicating a stronger user preference for the movie. 
If $\bw_n$ and $\bz_p$ fall sufficiently ``close," the value $\bw_n^\top\bz_p$ becomes larger. 
This concept elucidates the essence of ALS, where $\bw_n$ represents the features or attributes of movie $n$, while $\bz_p$ encapsulates  the features or preferences of user $p$. 
In other words,  ALS associates each user with a \textit{latent vector of preference} and each movie with a \textit{latent vector of attributes}.
Furthermore, each element in $\bw_n$ and $\bz_p$ signifies a specific feature. For example, it could be that the second feature $w_{n2}$ ($w_{n2}$ denotes the second element of vector $\bw_n$) represents whether the movie is an action movie or not, and $z_{p2}$ might denote whether  user $p$ has a preference for action movies. When this holds true, then $\bw_n^\top\bz_p$ becomes large and provides a good approximation of  $x_{np}$.

In the decomposition $\bX\approx \bW\bZ$, it is established that the rows of $\bW$ contain the hidden features of the movies, and the columns of $\bZ$ contain the hidden features of the users. 
Nevertheless, the explicit meanings of the rows in $\bW$ or the columns in $\bZ$ remain undisclosed.
Although they might correspond to categories or genres of the movies, fostering underlying connections between users and movies, their precise nature remains uncertain.
It is precisely this ambiguity that gives rise to the terminology ``latent" or ``hidden."

\begin{problemset}
\item Prove Corollary~\ref{corollary:ls_equiv} rigorously.

\item Prove Theorems~\ref{theorem:jensens_ineq} and \ref{theorem:conv_gradient_ineq} for convex functions.

\item Prove \eqref{equation:inv_aug_mat}.
%\item Prove Lemma~\ref{lemma:als-update-w-rank}.

%\item Find the update of column-by-column fashion for Algorithm~\ref{alg:als-regularizer}.

\item Determine all the minimizers  in Corollary~\ref{corollary:matrix_ls}.

\item \label{prob:inner_zprod} Given a positive definite matrix $\bZ$, show that the inner productg $\innerproduct{\bx,\by}_{\bZ} = \bx^\top\bZ\by$ for all $\bx,\by\in\real^n$ is a valid inner product satisfying Definition~\ref{definition:inner_prod}.

\item \textbf{Constrained (Regularized) least squares (CLS).} 
Given $\bX\in\real^{n\times p}, \by\in\real^{n}, \bY\in\real^{m\times p}$, and $\lambda\in\real_{++}$, we consider the constrained least squares problem:
$$
\mathop{\min}_{\bbeta\in\real^p} \normtwo{\bX\bbeta-\by}^2 + \lambda\normtwo{\bY\bbeta}^2.
$$
Show that the constrained least squares (CLS) problem has a unique solution if and only if $\nspace(\bX)\cap \nspace(\bY) = \{\bzero\}$.

\item \label{problem:wei_ls} \textbf{Weighted least squares (WLS).}
Going further from the assumptions in Theorem~\ref{theorem:ols}, we consider further that each data point $i\in\{1,2,\ldots, n\}$ (i.e., each row of $\bX$) has a weight $w_i$.
This means some  data points may carry greater significance than others and there are ways to produce approximate minimzers that reflect this.
Show that the value $\bbeta_{WLS} = (\bX^\top\bW^2\bX)^{-1}\bX^\top\bW^2\by$ serves as the \textit{weighted least squares (WLS)}  estimate of $\bbeta$, where $\bW=\diag(w_1, w_2, \ldots, w_n)\in\real^{n\times n}$. How is it related to the WLS we introduced in Section~\ref{section:generalizedLS}? \textit{Hint: find the normal equation for this problem.}

\item \label{prob:rls} \textbf{Restricted least squares (RLS).}
Going further from the assumptions in Theorem~\ref{theorem:ols}, we consider further the restriction $\bbeta=\bC\bgamma+\bc$, where $\bC\in\real^{p\times k}$ is a known matrix such that $\bX\bC$ has full rank, $\bc$ is a known vector, and $\bgamma$ is an unknown vector.
Show that the value $\bbeta_{RLS}=\bC(\bC^\top\bX^\top\bX\bC)^{-1}(\bC^\top\bX^\top)(\by-\bX\bc) +\bc$ serves as the \textit{restricted least squares (RLS)} estimate of $\bbeta$.

\item Find the restricted weighted least squares estimate.

%\index{Fermat's theorem}
%\item \label{problem:fist_opt} \textbf{First-order optimality condition for local optima points.} 
%Consider  \textit{Fermat's theorem}: for a one-dimensional function $g(\cdot)$ defined and differentiable over an interval ($a, b$), if a point $\beta^\star\in(a,b)$ is a local maximum or minimum, then $g^\prime(\beta^\star)=0$. 
%Prove the first-order optimality conditions for multivariate functions based on this Fermat's theorem for one-dimensional functions.
%That is, consider function $f: \sS\rightarrow \real$ as a function defined on a set $\sS\subseteq \real^p$. Suppose that $\bbeta^\star\in\text{int}(\sS)$, i.e., in the interior point of the set, is a local optimum point and that all the partial derivatives (Definition~\ref{definition:partial_deri}) of $f$ exist at $\bbeta^\star$. Then $\nabla f(\bbeta^\star)=\bzero$, i.e., the gradient vanishes at all local optimum points. (Note that, this optimality condition is a necessary condition; however, there could be vanished points which are not local maximum or minimum point.)

\item \label{problem:pos_hessian} \textbf{Global minimum point.} Let function $f$ be a twice continuously differentiable function defined over $\real^p$. Suppose that the Hessian $\nabla^2f(\bbeta) \geq 0$ for any $\bbeta\in\real^p$. Then $\bbeta^\star$ is a global minimum point of $f$ if $\nabla f(\bbeta^\star)=\bzero$. \textit{Hint: use linear approximation theorem in Theorem~\ref{theorem:linear_approx}.}

%Charu
\item \textbf{Two-sided matrix least squares.} Let $\bY$ be an $n\times k$ matrix and $\bZ$ be an $m\times p$ matrix. Find the $k\times m$ matrix $\bA$ such that $L(\bA)=\normf{\bX - \bY\bA\bZ}^2$ is minimized, where $\bX\in\real^{n\times p}$ is known. 
\begin{itemize}
\item Derive the derivative of $L$ with respect to $\bA$ and the optimality conditions. 
\item Show that one possible solution to the optimality conditions is $\bA=\bY^+\bX\bZ^+$, where $\bY^+$ and $\bZ^+$ are the pseudo-inverses of $\bY$ and $\bZ$, respectively.
\end{itemize}

\item \label{prob:statis_lev_semior} Let $\bQ\in\real^{n\times p}$ be any semi-orthogonal matrix whose columns span the column space of $\bX\in\real^{n\times p}$. Show that $h_{ii}$ of  $\bH = \bX(\bX^\top\bX)^{-1}\bX^\top$ can be obtained by $h_{ii} = \normtwo{\bq^{(i)}}^2$, where $\bq^{(i)}$ is the $i$-th row of $\bQ$. \textit{Hint: Use Theorem~\ref{theorem:orthogo_genspa}.}

\item Prove the Eckart-Young-Mirsky theorem w.r.t. Frobenius norm in Theorem~\ref{theorem:young-theorem_frob}. And show that this result also applies to the spectral norm (Definition~\ref{definition:spectral_norm}).

\item Prove the Hessian matrices \eqref{equation:als_hess_reg} in the ALS algorithm are positive definite after introducing regularizations.

%% below are GLS problems
\item Given a positive definite matrix $\bZ\in\real^{n\times n}$, show that  $\langle \bx, \by\rangle_{\bZ} = \bx^\top\bZ\by$ for all $\bx,\by\in\real^n$ is a valid norm that satisfies the three criteria in Definition~\ref{definition:inner_prod}. 

\end{problemset}

\newpage
\chapter{Numerical Methods for Least Squares Approximations}\label{chapter:ls_approx_num}
\begingroup
\hypersetup{
linkcolor=structurecolor,
linktoc=page,  % page: only the page will be colored; section, all, none etc
}
\minitoc \newpage
\endgroup

\section{General Ideas and Indirect Methods}\label{section:gradient-descent-all} 
\lettrine{\color{caligraphcolor}W}
When considering how long it takes to solve least squares (LS) problems, one can choose between two main types of methods: direct methods and indirect (or iterative) methods. Below is a brief overview of \textit{direct methods} for solving LS problems:
\begin{itemize}
\item \textit{Cholesky decomposition.} If the design matrix $\bX\in\real^{n\times p}$ has full column rank and is well-conditioned, then one can use the Cholesky decomposition to compute an upper triangular matrix $\bR$ such that $\bX^\top\bX = \bR^\top \bR$.
Once this decomposition is obtained, the normal equation $\bR^\top\bR \bbeta = \bX^\top \by$ can be solved efficiently.
\item \textit{QR decomposition.} 
Slightly slower but more numerically stable than Cholesky---especially when $\bX$ is rank-deficient or ill-conditioned---this method involves computing the QR decomposition $\bX = \bQ\bR$, where $\bQ$ is orthogonal and $\bR$ is upper triangular. The LS solution is then obtained by solving: $\bR \bbeta = \bQ^\top \by$.
\item \textit{SVD.} 
More computationally expensive but even more stable---particularly for very ill-conditioned matrices---the SVD computes:
$\bX = \bU\bSigma \bV^\top$, where this is the reduced SVD. 
The LS solution is then given by:
$
\widehat{\bbeta} = \bV \bSigma^{-1} \bU^\top \by.
$
For further details, see Section~\ref{section:ls-via-svd}.
\end{itemize}

The computational complexity of all these direct methods is $\mathcalO(np^2)$. That is, although the numerical stability and constant factors differ among the methods, all three classes of algorithms asymptotically require time proportional to $np^2$. In most cases, using QR decomposition offers a good balance between speed and stability. 

Another broad class of algorithms for solving LS and other problems are \textit{iterative methods}, among which gradient descent is the simplest example.

In this chapter, we will introduce approaches based on indirect (iterative) methods, as well as those using Cholesky and QR decompositions, including their computational aspects. The solution via SVD was already discussed in Section~\ref{section:ls-via-svd}. Computing the SVD requires more background knowledge and will not be covered here. For more information, refer to, for example, \citet{golub2013matrix}, \citet{lu2021numerical}, and \citet{bjorck2024numerical}.

\subsection*{Indirect Methods and Gradient Descent}
The general ideas or called indirect methods for solving least squares problems are those using descent methods to obtain the solution.
The \textit{gradient descent} (GD) method is a specific type of descent method used to find the (local or global) minimum of a differentiable function, whether convex or non-convex. This function is commonly referred to as the \textit{cost function} (also known as the \textit{loss function} or \textit{objective function}).
It stands out as one of the most popular algorithms to perform optimization and by far the
most common way to optimize machine learning, deep learning, and various optimization problems. 
This is especially true for optimizing neural networks and transformer networks \citep{lecun2015deep, goodfellow2016deep, vaswani2017attention}. 
In the context of machine learning, the cost function measures the difference between a model’s predicted output and the actual output. Neural networks, transformer networks, and machine learning models in general seek to find a set of parameters $\btheta\in \real^n$ (also known as weights, coefficients) that optimize an objective function $f(\btheta)$. This is expressed as the unconstrained optimization problem (P1):
$$
\textbf{(P1)}:\qquad \text{Find}\quad \btheta^* = \mathop{\argmin}_{\btheta} f(\btheta).
$$
Denoting $t=1,2,\ldots$ as the iteration number, iterative methods generate a sequence of vectors:
\begin{equation}
\btheta^{(1)}, \btheta^{(2)}, \ldots, \btheta^{(T)} \in \dom(f)
\end{equation}
\footnote{Some texts denote the starting point as $\btheta^{(0)}$, but in this book, we use $\btheta^{(1)}$.}
such that as $T\rightarrow \infty$, the sequence converges to the optimal solution $\btheta^*$, and the objective function value $f(\btheta^{(T)})$ approaches the optimal minimum $f(\btheta^*)$, under certain mild conditions.
At each iteration $t$, an \textit{update step (or a descent step)} $\bh^\toptzero$ is applied to update the parameters. Denoting the parameters at the $t$-th iteration as $\btheta^\toptzero$,  the update rule  is given by:
\begin{equation}\label{equation:sg_update_rule}
(\textbf{GD update}):\qquad 
\btheta^\toptone \leftarrow \btheta^\toptzero + \bh^\toptzero.
\end{equation}

\index{Learning rate}
\index{Strict SGD}
\index{Mini-batch SGD}
\index{Local minima}
\paragrapharrow{Gradient descent.}
The most basic form of gradient descent is the \textit{vanilla update}, where the parameters move in the opposite direction of the gradient. This follows the \textit{steepest descent direction} since gradients are orthogonal to level curves (also known as level surfaces, see Lemma~\ref{lemm:direction-gradients}):
\begin{equation}\label{equation:gd-equaa-gene}
\bh^\toptzero = -\eta_t \bg^\toptzero\triangleq -\eta_t \nabla f(\btheta^\toptzero),
\end{equation}
where the positive value $\eta_t$ denotes the \textit{learning rate  (or stepsize, step length, step size)} that depends on specific problems.
The term $\bg^\toptzero\triangleq\nabla f(\btheta^\toptzero) \in \real^n$ represents the gradient of the parameters.
The {learning rate} $\eta_t$ controls how large of a step to take in the direction of negative gradient so that we can reach a (local) minimum.
The method that follows the negative gradient direction (i.e., $\bd^\toptzero \triangleq -\nabla f(\btheta^\toptzero)$ in Algorithm~\ref{alg:struc_gd_gen}) is called the \textit{steepest descent method (or gradient method)}. 
The choice of descent direction is ``the best" (locally; see \eqref{equation:steep_des}) and we could combine it with an exact line search to determine the learning rate \citep{lu2025practical}. A method like this converges, but the final convergence is linear and often very slow. 

Examples in \citet{madsen2010and, boyd2004convex} show how the gradient descent method with exact line search and finite computer precision can fail to find the minimizer of a second degree polynomial. However, for many problems, it performs well in the early stages of the iterative process.
Considerations like this has lead to the so-called \textit{hybrid methods}, which---as the name suggests---are based on two different methods. One which is good in the initial stage, like the \textit{gradient method}, and another method which is good in the final stage, like \textit{Newton's method}. A key challenge with hybrid methods is designing an effective mechanism to switch between the two approaches at the appropriate time.

In \eqref{equation:sg_update_rule}, $\bh^\toptzero$ is referred to as a \textit{descent step}. While a direction $\bd^\toptzero$ satisfying the condition $\innerproduct{\bg^\toptzero,\bd^\toptzero}<0$ is called a \textit{descent direction}. In most cases, the relationship between the descent step and descent direction follows a scale by the learning rate:
\begin{equation}\label{equation:desce_step_direc}
\begin{aligned}
&\textbf{(Descent direction)}: \qquad&&\bd^\toptzero = -\bg^\toptzero;\\
&\textbf{(Descent step)}:\qquad &&\bh^\toptzero=\eta_t\bd^\toptzero.
\end{aligned}
\end{equation}
In many cases, when the learning rate is equal to 1, the above two terms are used \textbf{interchangeably},  then the descent direction and the descent step are the same; for example, in Newton's method.

\index{Newton's method}

\paragrapharrow{Gradient descent by calculus}
An intuitive analogy to understand gradient descent is to imagine the path of a river starting from a mountain peak and flowing downhill to reach the lowest point at its base.
Similarly, the goal of gradient descent is to find the lowest point in the landscape defined by the objective function $f(\btheta)$, where $\btheta$ represents  a $n$-dimensional input variable. Our task is to use an algorithm that guides us to a (local) minimum of $f(\btheta)$.
To better understand this process, consider moving a ball a small distance $h_1$ along the $\theta_1$ axis, a small amount $h_2$ along the $\theta_2$ axis, and so on up to $h_n$ along the $\theta_n$ axis. 
Calculus informs us of the variation in the objective function  $f(\btheta)$ as follows:
$$
\Delta f(\btheta) \approx \frac{\partial f}{\partial \theta_1}h_1 + \frac{\partial f}{\partial \theta_2}h_2 + \ldots + \frac{\partial f}{\partial \theta_n}h_n.
$$
Our challenge is to choose $h_1, h_2, \ldots, h_n$ such that they cause $\Delta f(\btheta)$ to be negative, thereby decreasing the objective function towards minimization.
Let $\bh=[h_1,h_2, \ldots , h_n]^\top$ denote  the vector of  changes in $\btheta$, and let $\nabla f(\btheta)=[\frac{\partial f}{\partial \theta_1},\frac{\partial f}{\partial \theta_2}, \ldots , \frac{\partial f}{\partial \theta_n}]^\top$ denotes the gradient vector of $f(\btheta)$ \footnote{Note the difference between $\Delta f(\btheta)$ and $\nabla f(\btheta)$.}. Then it follows that
$$
\Delta f(\btheta) \approx \frac{\partial f}{\partial \theta_1}h_1 + \frac{\partial f}{\partial \theta_2}h_2 +\ldots  + \frac{\partial f}{\partial \theta_n}h_n = \innerproduct{\nabla f(\btheta), \bh}.
$$ 
In the context of descending the function, our aim is to ensure that $\Delta f(\btheta)$ is negative. 
This ensures that moving from $\btheta^\toptzero$ to $\btheta^\toptone = \btheta^\toptzero+\bh^\toptzero$  (from $t$-th iteration to $(t+1)$-th iteration) results in a reduction of the loss function $f(\btheta^\toptone) = f(\btheta^\toptzero) + \Delta f(\btheta^\toptzero)$, given that $\Delta f(\btheta^\toptzero) \leq 0$.
It can be demonstrated that if the update step is defined as $\bh^\toptzero=-\eta_t \nabla f(\btheta^\toptzero)$, where $\eta_t$ is the learning rate, the following relationship holds:
$$
\Delta  f(\btheta^\toptzero) \approx -\eta_t \nabla f(\btheta^\toptzero)^\top\nabla f(\btheta^\toptzero) = -\eta_t\normtwobig{\nabla f(\btheta^\toptzero)}^2 \leq 0. 
$$
Specifically,  $\Delta  f(\btheta^\toptzero) < \bzero$ unless we are already at the optimal point with zero gradients.
This analysis validates the approach of gradient descent.

\index{Convex functions}
\paragrapharrow{Gradient descent for convex functions.}
We further explore the application of gradient descent in (unconstrained) convex problems.
If the objective function $f(\btheta)$ is (continuously differentiable) convex, then the relationship $\innerproduct{\nabla f(\btheta^\toptzero), (\btheta^\toptone-\btheta^\toptzero)}\geq 0$ implies $f(\btheta^\toptone) \geq f(\btheta^\toptzero)$. This can be derived from the gradient inequality of a continuously differentiable convex function, i.e., $f(\btheta^\toptone)- f(\btheta^\toptzero)\geq \innerproduct{\nabla f(\btheta^\toptzero), (\btheta^\toptone-\btheta^\toptzero)}$; see Theorem~\ref{theorem:conv_gradient_ineq}. 

In this sense, to ensure a reduction in the objective function from the point $\btheta^\toptzero$ to $\btheta^\toptone$, it is imperative to ensure  $\innerproduct{\nabla f(\btheta^\toptzero), (\btheta^\toptone-\btheta^\toptzero)}\leq 0$. 
In the context of  gradient descent, the choice of $\eta_t\bd^\toptzero = \btheta^\toptone-\btheta^\toptzero$ aligns with the negative gradient $-\nabla f(\btheta^\toptzero)$. However, there are many other descent methods, such as \textit{(non-Euclidean) greedy descent}, \textit{normalized steepest descent}, \textit{Newton step}, and so on. The core principle of these methods is to ensure that $\innerproduct{\nabla f(\btheta^\toptzero), (\btheta^\toptone-\btheta^\toptzero)}= \innerproduct{\nabla f(\btheta^\toptzero), \eta_t \bd^\toptzero} \leq 0$, provided the objective function is convex.

\paragrapharrow{Gradient descent with momentum.}
\textit{Gradient descent with momentum} is an improvement over basic gradient descent, frequently used in machine learning and deep learning to minimize the loss function and update model parameters. While standard gradient descent updates parameters solely based on the current gradient, momentum-based gradient descent introduces a \textit{momentum} term to accelerate convergence and smooth the optimization path.

In this approach, the momentum term enables the algorithm to build velocity in directions with a steady but small gradient, helping it overcome local minima and saddle points. By incorporating a fraction of the previous update into the current one, this technique mimics inertia, allowing the algorithm to continue moving in the same direction despite minor fluctuations in the gradient. Consequently, this method not only speeds up convergence but also reduces oscillations, particularly in regions where the surface curvature varies significantly across different dimensions.
At each iteration $t$, the process involves two key steps:
\begin{subequations}\label{equation:gd_with_momentum}
\begin{align}
\textbf{(Velocity update)}:\qquad \bd^\toptzero &\leftarrow \rho \bd^\toptminus - \eta_t \nabla f(\btheta^\toptzero);\\
\textbf{(Parameter update)}:\qquad \btheta^\toptone &\leftarrow \btheta^\toptzero + \bd^\toptzero.
\end{align}
\end{subequations}
By incorporating past gradients into the update rule, gradient descent with momentum enables more efficient traversal across the error surface, particularly in complex landscapes, leading to faster convergence and improved performance.
In summary, the gradient descent with momentum approach is advantageous for the following reasons \citep{lu2025practical}:
\begin{itemize}
\item At saddle points, the gradient of the cost function becomes nearly zero or entirely negligible. This results in minimal or no updates to the weights, causing the  learning process to stagnate and effectively halt.
\item The trajectory taken by the gradient descent method tends to be quite erratic, even when employing mini-batch processing. This jittery path can impede efficient convergence towards the minimum.
\end{itemize}

\paragrapharrow{Steepest descent.}
The linear approximation theorem (Theorem~\ref{theorem:linear_approx}) states that 
\begin{equation}\label{equation:gd_gree_tay11}
f(\btheta^\toptzero + \eta \bd) = f(\btheta^\toptzero)+ \eta \bd^\top \nabla  f(\btheta^\toptzero )+ \mathcalO(\normtwo{\eta \bd}^2).
\end{equation}
From \eqref{equation:gd_gree_tay11} and by the definition of directional derivative, we observe that when taking  a step $\eta \bd$ with a positive stepsize $\eta$,  the relative reduction in function value satisfies
$$
\lim_{\eta \rightarrow 0} \frac{f(\btheta^\toptzero) - f(\btheta^\toptzero + \eta \bd)}{\eta \normtwo{\bd}} = -\frac{1}{\normtwo{\bd}} \bd^{\top} \nabla f(\btheta^\toptzero) = \normtwobig{\nabla f(\btheta^\toptzero)} \cos (\phi),
$$
where $\phi$ is the angle between the vectors $\bd$ and $-\nabla f(\btheta^\toptzero)$. This equation indicates that we get the greatest gain rate if $\phi = 0$, meaning the optimal descent direction is the steepest descent direction $\bd_{\text{sd}}^\toptzero$, given by
\begin{equation}\label{equation:steep_des}
\bd_{\text{sd}}^\toptzero = -\nabla f(\btheta^\toptzero).
\end{equation}
That is, the steepest descent method coincides with the gradient descent method.

\paragrapharrow{Stochastic gradient descent.} 
In many cases, the function $f(\btheta)$ is defined over a datasets $\mathcalD=\{\bs_1, \bs_2, \ldots, \bs_D\}$ such that $f(\btheta)$ and its gradient $\nabla f(\btheta)$ can be expressed as 
\begin{equation}
f(\btheta)\triangleq f(\mathcalD; \btheta) = \frac{1}{D} \sum_{d=1}^{D} f(\bs_d; \btheta)
\qquad \text{and}\qquad 
\nabla f(\btheta) \triangleq \frac{1}{D}  \sum_{d=1}^{D}\nabla f(\bs_d; \btheta),
\end{equation} 
respectively.
While if we follow the negative gradient of a single sample or a batch of samples iteratively, the local estimate of the direction can be obtained and is known as the \textit{stochastic gradient descent} (SGD) \citep{robbins1951stochastic}.
The SGD method can be categorized  into two types:
\begin{itemize}
\item \textbf{The strict SGD:} Computes the gradient using only one randomly selected data point per iteration: $\nabla f(\btheta^\toptzero) \approx \nabla f(\bs_d; \btheta^\toptzero)$.
\item \textbf{The mini-batch SGD:} A compromise between full gradient descent and strict SGD, where a small subset (mini-batch) of the dataset is used to compute an estimate of the gradient: $\nabla f(\btheta^\toptzero) \approx \frac{1}{\abs{\sS}}  \sum_{d\in\sS}\nabla f(\bs_d; \btheta^\toptzero)$.
\end{itemize}
The SGD method is particular useful when the number of \textit{training entries} (i.e., the data used for updating/training the model, while the data used for final evaluation is called the \textit{test entries or test data}) are substantial, as computing the full gradient can be computationally expensive or even resulting in that the gradients from different input samples may cancel out and the final update is small.
However, since the gradient is estimated using only a subset of the data, the updates can be noisy.
In the SGD framework, the objective function is stochastic, composed of a sum of subfunctions evaluated at different subsamples of the data. However, a drawback of the vanilla update (both GD and SGD) lies in its susceptibility to getting trapped in local minima \citep{rutishauser1959theory}.

\paragrapharrow{Choice of stepsize.}
For a small stepsize, gradient descent ensures a monotonic improvement at every iteration, guaranteeing convergence, albeit to a local minimum. However, the speed of the vanilla gradient descent method is generally slow, and it can exhibit a linear rate in case of poor curvature conditions. 
While choosing a stepsize larger than an optimal threshold may cause divergence in terms of the objective function.
Determining an optimal learning rate (whether global or per-dimension) becomes more of an art than science for many problems. 
Previous work has attempted to alleviate the need for manually selecting a global learning rate \citep{zeiler2012adadelta, lu2023adasmooth}, though such methods remain sensitive to other hyper-parameters. Alternatively,  \textit{(exact or inexact) line search strategies} can be applied to determine the stepsize more systematically \citep{lu2025practical}.

\paragrapharrow{Descending Property.}
Most (if not all) optimization methods incorporate mechanisms to enforce the descending property:
\begin{equation}\label{equation:des_prob1}
f(\btheta^\toptone) < f(\btheta^\toptzero). 
\end{equation}
This prevents convergence to a maximizer and also makes it less probable that we get convergence to a \textit{saddle point} (a vanishing gradient point that is neither a local minimum point nor a maximum point of the cost function). 
If the objective function has several minimizers, the final solution depends on the starting point $\btheta^{(1)}$. We do not know which of the minimizers that will be found; the specific minimizer found is not necessarily the one closest to $\btheta^{(1)}$.

As mentioned previously, in many cases the method produces vectors which converge towards the minimizer in two clearly different stages: the ``global stage" where $\btheta^{(1)}$ is far from the solution and we want the method to produce iterates which move steadily towards the optimizer $\btheta^*$, and the ``final stage" where $\btheta^\toptzero$ is close to $\btheta^*$ and seek faster convergence.

The global convergence properties of a method describe its behavior when initialization occurs at a point $\btheta^{(1)}$, which is not close to a (local) minimizer $\btheta^{*}$. 
Ideally, the iterates should move steadily toward a neighborhood of  $\btheta^{*}$. For instance, there are methods for which it is possible to prove that any accumulation point (i.e., limit of a subseries) of $\{\btheta^\toptzero\}_{t>0}$ is a stationary point, meaning the gradient vanishes:
$$
\nabla f(\btheta^\toptzero) \rightarrow \bzero \qquad \text{ for } \qquad t \rightarrow \infty. 
$$
While this does not eliminate the possibility of convergence to a saddle point or maximizer, the descending property \eqref{equation:des_prob1} typically prevents such cases in practice.
In this ``global phase", our primary concern is ensuring that losses do not increase (except for possibly the initial steps).
To analyze convergence in terms of iterates rather than function values, a natural potential function is 
$$
\text{$e_t\triangleq\normtwobig{\be^\toptzero},\quad$ where $\be^\toptzero\triangleq \btheta^\toptzero - \btheta^*$.}
$$
Let $\{\be^\toptzero\}_{t>0}$ denote the error sequence. The requirement for progress is:
$$
\normtwobig{\be^\toptone} < \normtwobig{\be^\toptzero}\qquad  \text{ for }\qquad t > t'.
$$

In the final stages of the iteration where the $\btheta^\toptzero$ are close to $\btheta^{*}$, we expect faster convergence. 
Local convergence analysis describes how quickly the iterates approach $\btheta^{*}$ to a desired accuracy. Some methods exhibit  linear convergence:
$$
\normtwobig{\be^\toptone}\leq c_{1}\normtwobig{\be^\toptzero}, \quad \text{with } 0 < c_{1} < 1 \text{ and } \btheta^\toptzero \text{ close to } \btheta^{*}.  
$$
However, higher-order convergence is preferable. For instance, quadratic convergence satisfies:
$$
\normtwobig{\be^\toptone}\leq c_{2}\normtwobig{\be^\toptzero}^{2}, \quad \text{with } c_{2} > 0 \text{ and } \btheta^\toptzero \text{ close to } \btheta^{*}. 
$$
Few practical methods achieve quadratic convergence, but superlinear convergence  is a common goal:
$$
{\normtwobig{\be^\toptone}}/{\normtwobig{\be^\toptzero}} \rightarrow 0 \quad\text{ for } t \rightarrow \infty.
$$
Superlinear convergence is faster than linear convergence, though typically not as rapid as quadratic convergence.

\paragrapharrow{Framework of a descent method.}

\begin{algorithm}[H] 
\caption{Structure of  Descent Methods}
\label{alg:struc_gd_gen}
\begin{algorithmic}[1] 
\Require A function $f(\btheta)$; 
\State {\bfseries Input:}  Initialize $\btheta^{(1)}$;
\For{$t=1,2,\ldots$}
\State Find a descent direction $\bd^\toptzero$ such that $\innerproduct{\bd^\toptzero, \bg^\toptzero}<0$;
\State Pick a stepsize $\eta_t$;
\State $\btheta^{(t+1)} \leftarrow \btheta^\toptzero + \eta_t \bd^\toptzero$;
\EndFor
\State (Output Option 1) Output  $\btheta_{\text{final}}\leftarrow \btheta^{(T)}$;
\State (Output Option 2) Output  $\btheta_{\text{avg}}\leftarrow \frac{1}{T}(\sum_{t=1}^{t}\btheta^\toptzero)$ or $\sum_{t=1}^{T} \frac{2t}{T(T+1)} \btheta^\toptzero$;
\State (Output Option 3) Output  $\btheta_{\text{best}}\leftarrow \argmin_{t\in\{1,2,\ldots,T\}} f(\btheta^\toptzero)$;
\end{algorithmic} 
\end{algorithm}

The methods presented in this book are descent methods, meaning they satisfy the descending condition \eqref{equation:des_prob1} at each iteration. Each iteration consists of:
\begin{itemize}
\item Finding a descent direction $\bd^\toptzero$ at the $t$-th iteration.
\item Determining a stepsize $\eta_t$ giving a good decrease in the function value.
\end{itemize}
This sequence of operations forms the foundation of descent algorithms, see Algorithm~\ref{alg:struc_gd_gen}.
The search direction $ \bd^\toptzero $ at each iteration must be a descent direction. 
This ensures that we can reduce $ f(\btheta) $ by choosing an appropriate walking distance, and thus we can satisfy the descending condition \eqref{equation:des_prob1}.

\begin{exercise}[GD for LS]
Derive the gradient descent methods for OLS (see \eqref{equation:ls_l2norm}), GLS (see \eqref{equation:gls_prob_loss}), and the augmented LS problems (see \eqref{equation:sys_aug_sys}).
\end{exercise}

\paragrapharrow{Stopping criteria.}
Ideally, a stopping criterion should indicate when the current error is sufficiently small:
$$
\textbf{(ST1)}:\qquad \normtwobig{\be^{\toptzero}} < \delta_{1}.
$$
Another ideal condition would be when the current function value is close enough to the minimum:
$$
\textbf{(ST2)}:\qquad f(\btheta^\toptzero ) - f(\btheta^{*}) < \delta_{2}.
$$
Both conditions reflect the convergence $ \btheta^\toptzero  \rightarrow \btheta^{*} $. 
However, they are impractical because $ \btheta^{*} $ and $ f(\btheta^{*}) $ are (in most cases) unknown~\footnote{In some cases, the value $f(\btheta^*)$ is known. For example, in the convex feasibility problem, we seek feasible points within a convex set, where $f(\btheta^*)$ is zero.}. 
Instead, we rely on approximations:
\begin{equation}\label{equation:des_stopcri1}
\textbf{(ST3)}:\qquad \normtwobig{\btheta^\toptone - \btheta^\toptzero} < \varepsilon_{1} \qquad \text{or} \qquad f(\btheta^\toptzero ) - f(\btheta^\toptone) < \varepsilon_{2}.
\end{equation}
We must emphasize that even if \eqref{equation:des_stopcri1} is fulfilled with small $ \varepsilon_{1} $ and $ \varepsilon_{2} $, it does not guarantee that  $ \normtwobig{\be^{\toptzero}} $ or $ f(\btheta^\toptzero ) - f(\btheta^{*}) $ are small.

Another form of convergence, mentioned earlier in this section, is $ \nabla f (\btheta^\toptzero ) \rightarrow \bzero $ for $ t \rightarrow \infty $. This leads to another commonly used stopping criterion:
\begin{equation}\label{equation:des_stopcri2}
\textbf{(ST4)}:\qquad \normtwobig{\nabla f(\btheta^\toptzero )} < \varepsilon_{3},
\end{equation}
which is included in many implementations of descent methods.

Another useful approach involves leveraging the property of converging function values. 
The quadratic approximation (Theorem~\ref{theorem:quad_app_theo}) of $ f $ at $ \btheta^{*} $ is
$$
f(\btheta^\toptzero ) \approx f(\btheta^{*}) + (\btheta^\toptzero  - \btheta^{*})^{\top} \nabla f(\btheta^{*}) + \frac{1}{2}(\btheta^\toptzero  - \btheta^{*})^{\top} \nabla^2 f(\btheta^{*})(\btheta^\toptzero - \btheta^{*}).
$$
Since $ \btheta^{*} $ is a local minimizer, we have $ \nabla f(\btheta^{*}) = \bzero $ and $ \bH^{*} \triangleq \nabla^2 f(\btheta^{*}) $ is positive semidefinite. This simplifies to:
$
f(\btheta^\toptzero ) - f(\btheta^{*}) \approx \frac{1}{2}(\btheta^\toptzero  - \btheta^{*})^{\top} \bH^{*}(\btheta^\toptzero  - \btheta^{*}).
$
Thus, another stopping criterion can be defined as:
$$
\textbf{(ST5)}:\qquad \frac{1}{2}(\btheta^\toptone - \btheta^\toptzero )^{\top} \bH^\toptzero(\btheta^\toptone - \btheta^\toptzero ) < \varepsilon_{4} \quad \text{with} \quad \btheta^\toptzero  \approx \btheta^{*}. 
$$
Here, $ \btheta^\toptzero  - \btheta^{*} $ is approximated by $ \btheta^\toptone - \btheta^\toptzero  $ and $ \bH^{*} $ is approximated by
$
\bH^\toptzero \triangleq \nabla^2 f(\btheta^\toptzero ).
$

In the following sections, we delve into a detailed exploration of the gradient descent method, examining its variations and adaptations from different perspectives. This comprehensive analysis aims to provide a deeper understanding of the algorithm, its formulations, challenges, and practical applications.

\index{Greedy search}
\index{Non-Euclidean gradient descent}
\index{Dual norm}
\subsection{Gradient Descent by Greedy Search and Variants}\label{section:als-gradie-descent-taylor}

We now consider  the \textit{greedy search} method such that $\btheta^\toptone    \leftarrow \mathop{\arg \min}_{\btheta^\toptzero} f(\btheta^\toptzero)$ under some mild conditions. 
The linear approximation theorem (Theorem~\ref{theorem:linear_approx}) shows that 
\begin{equation}\label{equation:gd_gree_tay}
f(\btheta^\toptzero + \eta \bd) = f(\btheta^\toptzero)+ \eta \bd^\top \nabla  f(\btheta^\toptzero )+ \mathcalO(\normtwo{\eta \bd}^2).
\end{equation}
For small values of  $\eta$, the term $\mathcalO(\normtwo{\eta \bd}^2)$ becomes negligible compared to the middle term. Therefore, we can approximate $f(\btheta^\toptzero  + \eta \bd)$ as
\begin{equation}\label{equation:gd_gree_approx}
f(\btheta^\toptzero + \eta \bd) \approx f(\btheta^\toptzero ) + \eta \bd^\top \nabla  f(\btheta^\toptzero ),
\end{equation}
when $\eta$ is sufficiently small. 
The second term on the right-hand side, $ \bd^\top \nabla  f(\btheta^\toptzero ) $, is the \textit{directional derivative} of $ f $ at $ \btheta^\toptzero $ in the direction $ \bd $. 
To reiterate, it indicates the approximate change in $ f $ for a small step $ \bd $. The step $ \bd $ is a descent direction if the directional derivative is negative.

To address how to choose $\bd$ to make the directional derivative as negative as possible, note that since the directional derivative $ \bd^\top \nabla  f(\btheta^\toptzero ) $ is linear in $ \bd $, it can be made arbitrarily negative by increasing $ \bd $ (provided $ \bd $ is a descent direction, i.e., $ \bd^\top \nabla  f(\btheta^\toptzero ) < 0 $). To make this question meaningful, we must limit the size of $ \bd $, or normalize by its length.

Let $ \norm{\cdot} $ be any norm on $ \real^n $. We define a \textit{normalized greedy descent direction} (with respect to the norm $ \norm{\cdot} $) as
\begin{equation}\label{equation:norma_greedes}
\bd_{\text{ngd}}^\toptzero \in \argmin_{\bd} \left\{ \bd^\top \nabla  f(\btheta^\toptzero ) \text{ s.t. } \norm{\bd} = 1 \right\}.
\end{equation}
(Note that there may be multiple minimizers.) A normalized greedy descent direction $ \bd_{\text{ngd}}^\toptzero $ is a step of unit norm that provides the largest decrease in the linear approximation of $ f $.
By the definition of the dual norm \eqref{equation:dual_norm_equa}, it follows that $\innerproduct{\bd_{\text{ngd}}^\toptzero, \nabla  f(\btheta^\toptzero )} =-\norm{\nabla  f(\btheta^\toptzero )}_{*}$, where $\norm{\cdot}_{*}$ denotes the dual norm.

Since the problem in \eqref{equation:norma_greedes} can  equivalently be stated using the constraint $\norm{\bd}\leq 1$, $\bd_{\text{ngd}}^\toptzero$ also lies within the set of primal counterparts of $\nabla  f(\btheta^\toptzero ) $, whose existence is shown in Definition~\ref{definition:set_primal}.

It is also convenient to consider an unnormalized greedy descent step $\bd_{\text{ugd}}^\toptzero$ by scaling the normalized greedy descent direction in a particular way:
\begin{equation}\label{equation:unnorma_greedes}
\bd_{\text{ugd}}^\toptzero \triangleq \norm{\nabla f(\btheta^\toptzero)}_{*} \bd_{\text{ngd}}^\toptzero.
\end{equation}
The reason for this particular unnormalization is that it aligns with the negative gradient (the steepest descent direction) when the underlying norm is the $\ell_2$ norm.
Note that for the greedy descent step, we have
$$
\nabla f(\btheta^\toptzero)^\top \bd_{\text{ugd}}^\toptzero = \norm{\nabla f(\btheta^\toptzero)}_{*} \nabla f(\btheta^\toptzero)^\top \bd_{\text{ngd}}^\toptzero = -\norm{\nabla f(\btheta^\toptzero)}_{*}^2.
$$
However, when the exact line search method is used to find the learning rate,  scale factors in the descent direction do not affect the outcome, so either the normalized or unnormalized direction can be used.

\paragrapharrow{Greedy search for $\ell_2$ norm.}
If we take the norm $\norm{\cdot}$ to be the $\ell_2$ norm, we find that the greedy descent direction in \eqref{equation:norma_greedes} is simply the negative gradient, i.e., 
\begin{equation}\label{equation:greedy_l2_ngdugd}
\bd_{\text{ngd}}^\toptzero = - \frac{\nabla f(\btheta^\toptzero )}{\normtwobig{\nabla f(\btheta^\toptzero )}}
\qquad \text{and}\qquad 
\bd_{\text{ugd}}^\toptzero = -\nabla f(\btheta^\toptzero ).
\end{equation}
The greedy descent method for the $\ell_2$ norm coincides with the gradient descent method (or the steepest descent method).
The above equality also shows that the unnormalized greedy search direction  corresponds to the negative gradient direction or the steepest descent direction.

\paragrapharrow{Greedy search for $\ell_1$ norm.}
As another example,  consider the greedy descent method for the $\ell_1$ norm. A normalized greedy descent direction can be characterized as
$$
\bd_{\text{ngd}}^\toptzero \in \argmin_{\bd} \left\{ \bd^\top \nabla  f(\btheta^\toptzero ) \text{ s.t. } \normone{\bd} \leq 1 \right\}.
$$
We use `$\in$' since the solution of the problem may not be unique.
Let $i$ be any index for which $\norminf{\nabla f(\btheta^\toptzero)} = \abs{(\nabla f(\btheta^\toptzero))_i}$. 
By the definition of the $\ell_\infty$ norm and the dual norm (Definition~\ref{definition:vec_l2_norm} and \eqref{equation:dual_norm_equa}), a normalized greedy descent direction $\bd_{\text{ngd}}^\toptzero$ for the $\ell_1$ norm is given by
$$
\bd_{\text{ngd}}^\toptzero = -\sign\left(\frac{\partial f}{\partial \theta_i}(\btheta^\toptzero)\right) \be_i,
$$
where $\be_i$ is the $i$-th  unit basis vector (see Example~\ref{example:set_primal_count}). An unnormalized greedy descent step is then
$$
\bd_{\text{ugd}}^\toptzero = \bd_{\text{ngd}}^\toptzero \norminf{\nabla f(\btheta^\toptzero)} = -\frac{\partial f}{\partial \theta_i} (\btheta^\toptzero) \be_i.
$$
which is a descent direction since $\innerproduct{\bd_{\text{ugd}}^\toptzero, \nabla f(\btheta^\toptzero)}<0$ (assuming $\nabla f(\btheta^\toptzero) \neq\bzero$).
Thus, the normalized greedy descent step in the $\ell_1$ norm can always be chosen to be a (positive or negative) standard basis vector, representing the coordinate axis direction along which the approximate decrease in $f$ is greatest.
Note that the index for which $\norminfbig{\nabla f(\btheta^\toptzero)} = \absbig{(\nabla f(\btheta^\toptzero))_i}$ may not be unique (see Example~\ref{example:set_primal_count}). In such cases, a convex combination of these descent directions can be used as the final descent direction.

The greedy descent algorithm in the $\ell_1$ norm has a  natural interpretation: At each iteration we select a component of $\nabla f(\btheta^\toptzero)$ with maximum absolute value (though the component may not be unique), and then decrease or increase the corresponding component of $\btheta^\toptzero$, according to the sign of $(\nabla f(\btheta^\toptzero))_i$. The algorithm is sometimes called a \textit{coordinate-descent algorithm} because only one component of the variable $\btheta$ is updated at each iteration, potentially simplifying or even trivializing the line search.

\paragrapharrow{Greedy search for $\bQ$-norm.}
We further consider the $\bQ$-norm: 
\begin{equation}\label{equation:q_norm}
\norm{\btheta}_{\bQ} = (\btheta^\top \bQ \btheta)^{1/2} = \normtwo{\bQ^{1/2} \btheta}
\end{equation}
for any $\btheta\in\real^n$,
where $\bQ$ is positive definite. The normalized greedy descent direction is given by
$$
\begin{aligned}
\bd_{\text{ngd}}^\toptzero 
&=\argmin_{\bd} \left\{ \bd^\top \nabla  f(\btheta^\toptzero ) \text{ s.t. } \norm{\bd}_{\bQ} \leq 1 \right\}
=-\normtwo{\bQ^{-1/2} \nabla f(\btheta^\toptzero)}^{-1/2}\bQ^{-1} \nabla f(\btheta^\toptzero). 
\end{aligned}
$$
This can be solved using the KKT conditions or the definition of the dual norm.
The dual norm is given by $\norm{\btheta}_{*} = \normtwo{\bQ^{-1/2} \btheta} =\norm{\btheta}_{\bQ^{-1}}$ for any $\btheta\in\real^n$, so the greedy descent step with respect to $\norm{\cdot}_{\bQ}$ is given by
\begin{equation}\label{equation:qnorm_ugd}
\bd_{\text{ugd}}^\toptzero = -\bQ^{-1} \nabla f(\btheta^\toptzero).
\end{equation}
This is a descent direction; see Problem~\ref{prob:qnorm_des}.

\paragrapharrow{Change of variables in $\bQ$-norm.}

An interesting alternative interpretation of the greedy descent direction $\bd_{\text{ugd}}^\toptzero $ is as the gradient search direction after applying a change of coordinates to the problem. Let $\widetildebx \triangleq \bQ^{1/2} \btheta$; thus, $\norm{\btheta}_{\bQ} = \normtwo{\widetildebx}$. Using this change of coordinates, we can solve the original problem of minimizing $f$ by solving the equivalent problem of minimizing the function $\widetildef : \real^n \rightarrow \real$, given by
$$
\widetildef(\widetildebx) \triangleq f(\bQ^{-1/2} \widetildebx) = f(\btheta).
$$
If we apply the gradient method to $\widetildef$, the search direction at a point $\widetildebx^\toptzero$ (which corresponds to the point $\btheta^\toptzero = \bQ^{-1/2} \widetildebx^\toptzero$ for the original problem) is
$$
\widetildebd^\toptzero = -\nabla \widetildef(\widetildebx^\toptzero) = -\bQ^{-1/2} \nabla f(\bQ^{-1/2} \widetildebx^\toptzero) = -\bQ^{-1/2} \nabla f(\btheta^\toptzero).
$$
Since $\widetildebx = \bQ^{1/2} \btheta$ by definition, the search direction in the original space is obtained by mapping $\widetildebd^\toptzero$ back using $\bQ^{-1/2}$: 
$$
\bd^\toptzero = \bQ^{-1/2} \widetildebd^\toptzero = -\bQ^{-1} \nabla f(\btheta^\toptzero)
$$
which corresponds to the unnormalized greedy search direction in \eqref{equation:qnorm_ugd}. In other words, the greedy descent method in the $\bQ$-norm $\norm{\cdot}_{\bQ}$ can be thought of as the gradient method applied to the problem after the change of variables $\widetildebx^\toptzero = \bQ^{1/2} \btheta^\toptzero$ for each iteration $t$.

\begin{exercise}[Greedy descent for LS]
Derive the greedy descent methods with $\bQ$-norm for OLS (see \eqref{equation:ls_l2norm}), GLS (see \eqref{equation:gls_prob_loss}), and the augmented LS problems (see \eqref{equation:sys_aug_sys}).
\end{exercise}

\subsection{Geometrical Interpretation of Gradient Descent} 
\begin{lemma}[Direction of gradients]\label{lemm:direction-gradients}
An important fact is that gradients are orthogonal to level curves (also known as level surfaces).
\end{lemma}
\begin{proof}[of Lemma~\ref{lemm:direction-gradients}: Informal]
To prove this, we need to show that the gradient is orthogonal to the tangent of the level curve. Let's start with the two-dimensional case. Suppose the level curve has the form $f(x,y)=c$. 
This equation implicitly defines a relationship  between $x$ and $y$, such that $y=y(x)$, where $y$ can be considered as a function of $x$. Therefore, the level curve can be expressed as:
$$
f(x, y(x)) = c.
$$
Applying the chain rule gives us:
$$
\frac{\partial f}{\partial x} \underbrace{\frac{dx}{dx}}_{=1} + \frac{\partial f}{\partial y} \frac{dy}{dx}=0.
$$
This implies that the gradient is perpendicular to the tangent vector:
$$
\left\langle \frac{\partial f}{\partial x}, \frac{\partial f}{\partial y}\right\rangle
\cdot 
\left\langle \frac{dx}{dx}, \frac{dy}{dx}\right\rangle=0.
$$
Now, let's generalize this to higher dimensions. Consider a level set defined by a vector $\btheta\in \real^n$: $f(\btheta) = f(\theta_1, \theta_2, \ldots, \theta_n)=c$. Each variable $\theta_i$ can be regarded as a function of a parameter $t$ along the level set $f(\btheta)=c$: $f(\theta_1(t), \theta_2(t), \ldots, \theta_n(t))=c$. Differentiating both sides with respect to $t$ using the chain rule yields:
$$
\frac{\partial f}{\partial \theta_1} \frac{d\theta_1}{dt} + \frac{\partial f}{\partial \theta_2} \frac{d\theta_2}{dt}
+\ldots + \frac{\partial f}{\partial \theta_n} \frac{d\theta_n}{dt}
=0.
$$
Therefore, the gradients is perpendicular to the tangent in the $n$-dimensional case:
$$
\left\langle \frac{\partial f}{\partial \theta_1}, \frac{\partial f}{\partial \theta_2}, \ldots, \frac{\partial f}{\partial \theta_n}\right\rangle
\cdot 
\left\langle \frac{d\theta_1}{dt}, \frac{d\theta_2}{dt}, \ldots, \frac{d\theta_n}{dt}\right\rangle=0.
$$
This completes the proof.
\end{proof}
The lemma above provides a profound geometric interpretation of gradient descent. In the process of minimizing a (convex) function $f(\btheta)$, gradient descent strategically moves in the direction opposite to the gradient, which reduces the loss. Figure~\ref{fig:alsgd-geometrical} illustrates a two-dimensional scenario where  $-\nabla f(\btheta)$ guides the decrease in loss for a (convex) function $f(\btheta)$. 

\begin{figure}[h]
\centering   
\vspace{-0.25cm}  
\subfigtopskip=2pt  
\subfigbottomskip=2pt  
\subfigcapskip=-5pt  
\subfigure[A two-dimensional convex function $f(\btheta)$.]{\label{fig:alsgd1}
\includegraphics[width=0.47\linewidth]{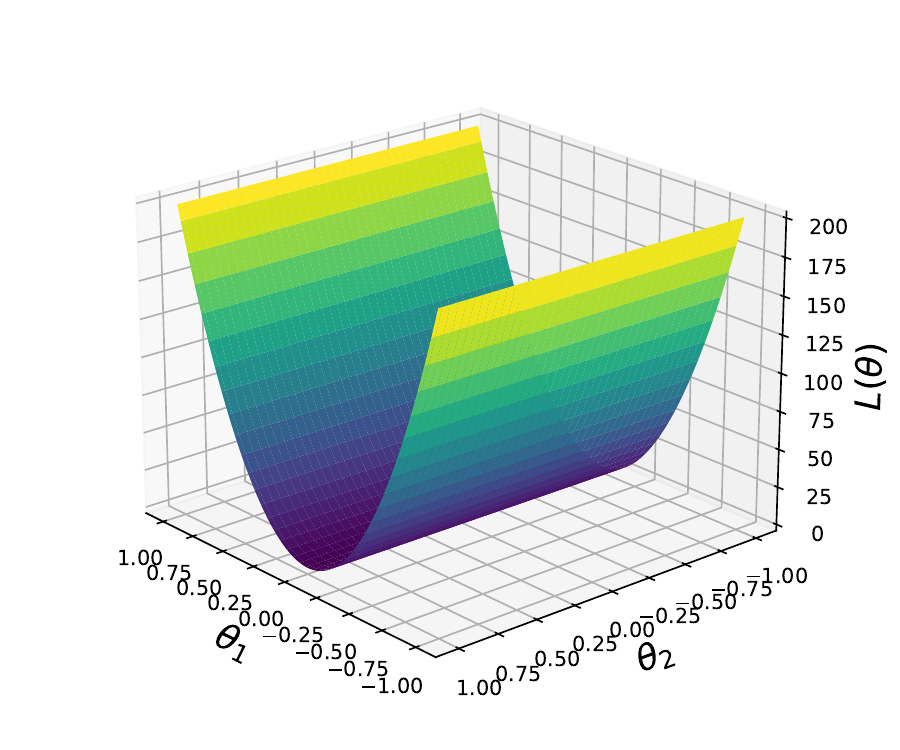}}
\subfigure[$f(\btheta)=c$ is a constant.]{\label{fig:alsgd2}
\includegraphics[width=0.44\linewidth]{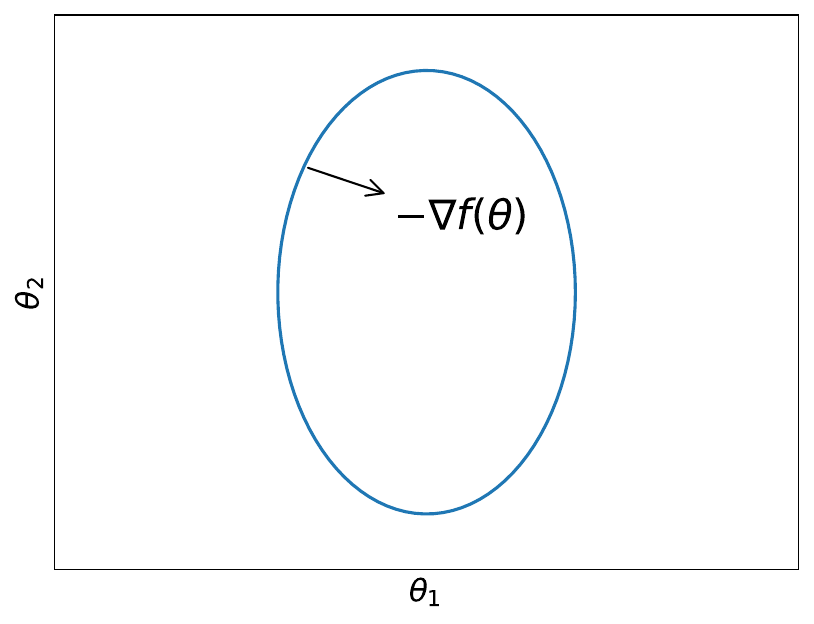}}
\caption{Figure~\ref{fig:alsgd1} shows a convex function surface plot and its contour plot (\textcolor{mylightbluetext}{blue}=low, \textcolor{mydarkyellow}{yellow}=high), where the upper graph represents  the surface plot, and the lower one is its projection (i.e., contour). Figure~\ref{fig:alsgd2}: $-\nabla f(\btheta)$ directs the reduction in loss for the convex function $f(\btheta)$.}
\label{fig:alsgd-geometrical}
\end{figure}

\index{Regularization}
\subsection{Geometrical Interpretation of Regularization}\label{section:geom_int_regu}
\begin{figure}[h]
\centering
\includegraphics[width=0.95\textwidth]{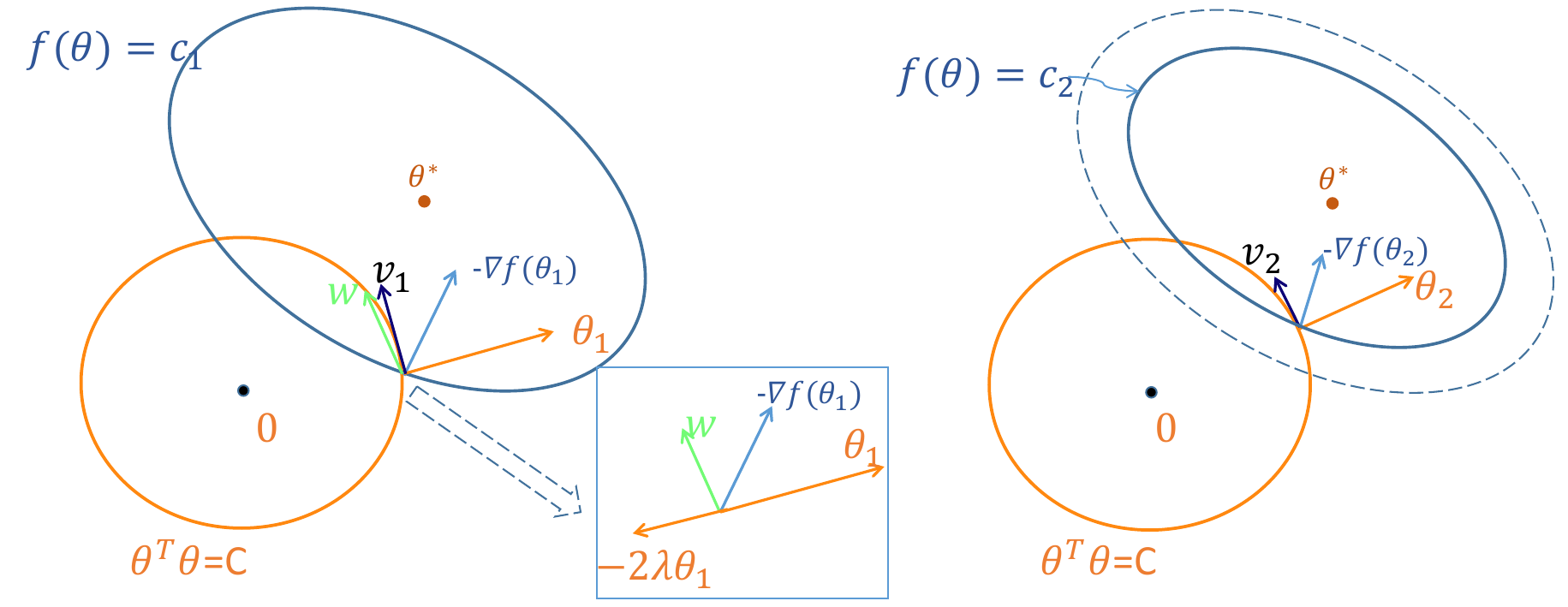}
\caption{Constrained gradient descent with $\btheta^\top\btheta\leq C$. The \textcolor{mydarkgreen}{green} vector $\bw$ is the projection of $\bv_1$ into $\btheta^\top\btheta\leq C$, where $\bv_1$ is the component of $-\nabla f(\btheta)$ perpendicular to $\btheta_1$. The right picture shows the next step after the update in the left picture. $\btheta^*$ denotes the optimal solution of \{$\min f(\btheta)$\}.}
\label{fig:alsgd3}
\end{figure}
\textit{Regularization} is a machine learning technique employed to prevent overfitting and improve model generalization; see Sections~\ref{section:totalles_otheriss} and  \ref{section:regularization-extention-general}. Overfitting occurs when a model is overly complex and fits the training data too closely, resulting in poor performance on  unseen data. 
To mitigate this issue, regularization introduces a constraint or a penalty term into the loss function used for model optimization, discouraging the development of overly complex models. 
This creates  a trade-off between having a simple, generalizable model and fitting the training data well. 
Common types of regularization include $\ell_1$ regularization, $\ell_2$ regularization (Tikhonov regularization), and elastic net regularization (a combination of $\ell_1$ and $\ell_2$ regularizations). 
Regularization finds extensive applications in machine learning algorithms such as linear regression, logistic regression, and neural networks.

Gradient descent also reveals the geometric significance of regularization. To avoid confusion, we denote the loss function without regularization by $f(\btheta)$ and the loss with the $\ell_2$ regularization by $F(\btheta) \triangleq f(\btheta)+\lambda \normtwo{\btheta}^2$, where $f(\btheta): \real^n \rightarrow \real$. 
When minimizing $f(\btheta)$, the descent method searches for a solution in $\real^n$. 
However, in machine learning, an exhaustive search across the entire space may lead to overfitting. A partial remedy involves searching within a subset of the vector space, such as searching in $\btheta^\top\btheta < C$ for some constant $C$. That is,
$$
\argmin_{\btheta} \, \big\{f(\btheta) \gap  \text{s.t.} \gap \btheta^\top\btheta\leq C\big\}.
$$
This constrained search helps prevent overfitting by introducing regularization through the addition of a penalty term in the optimization process.
In the previous discussion, a basic gradient descent approach proceeds in the direction of $-\nabla f(\btheta)$,  updating $\btheta$ by $\btheta^+\leftarrow \btheta-\eta \nabla f(\btheta)$ for a small stepsize $\eta$. 
When the level curve is $f(\btheta)=c_1$ and the descent approach is situated at $\btheta=\btheta_1$, where $\btheta_1$ is the intersection of $\btheta^\top\btheta=C$ and $f(\btheta)=c_1$, the descent direction $-\nabla f(\btheta_1)$ will be perpendicular to the level curve of $f(\btheta_1)=c_1$, as shown in the left picture of Figure~\ref{fig:alsgd3}. 
However, if we further restrict that the optimal value can only be in the subspace $\btheta^\top\btheta\leq C$, the trivial descent direction $-\nabla f(\btheta_1)$ will lead $\btheta_2=\btheta_1-\eta \nabla f(\btheta_1)$ outside of $\btheta^\top\btheta\leq C$. 

To address this,  the step $-\nabla f(\btheta_1)$ is decomposed  into 
$
-\nabla f(\btheta_1) = a\btheta_1 + \bv_1,
$
where $a\btheta_1$ is the component perpendicular to the curve of $\btheta^\top\btheta=C$, and $\bv_1$ is the component parallel to the curve of $\btheta^\top\btheta=C$. Keeping only the step $\bv_1$, then the update 
$$
\btheta_2 = \text{project}(\btheta_1+\eta \bv_1) = \text{project}\left(\btheta_1 + \eta
\underbrace{(-\nabla f(\btheta_1) -a\btheta_1)}_{\bv_1}\right)\footnote{where the project($\btheta$) operator
will project the vector $\btheta$ to the closest point inside $\btheta^\top\btheta\leq C$. Notice here the direct update $\btheta_2 = \btheta_1+\eta \bv_1$ can still make $\btheta_2$ outside the curve of $\btheta^\top\btheta\leq C$.}
$$ 
will lead to a smaller loss from $f(\btheta_1)$ to $f(\btheta_2)$ while still satisfying the constraint $\btheta^\top\btheta\leq C$. 
This technique is known as the \textit{projected gradient descent} \citep{beck2017first, lu2025practical}. It is not hard to see that the update $\btheta_2 = \text{project}(\btheta_1+\eta \bv_1)$ is equivalent to finding a vector $\bw$ (depicted by the \textcolor{mydarkgreen}{green} vector in the left panel of Figure~\ref{fig:alsgd3}) such that $\btheta_2=\btheta_1+\bw$ lies inside the curve of $\btheta^\top\btheta\leq C$. Mathematically, the vector $\bw$ can be obtained as $-\nabla f(\btheta_1) -2\lambda \btheta_1$ for some $\lambda$, as shown in the middle panel of Figure~\ref{fig:alsgd3}. This aligns with the negative gradient of $F(\btheta)=f(\btheta)+\lambda\normtwo{\btheta}^2$ such that 
$$
-\nabla F(\btheta_1) = -\nabla f(\btheta_1) - 2\lambda \btheta_1,
$$
and 
$$
\begin{aligned}
\bw &= -\nabla F(\btheta_1) 
\qquad \implies \qquad 
\btheta_2 = \btheta_1+ \bw =\btheta_1 -  \nabla F(\btheta_1).
\end{aligned}
$$
In practice, a small stepsize $\eta$ can be applied to prevent crossing  the curve boundary of $\btheta^\top\btheta\leq C$:
$$
\btheta_2  =\btheta_1 -  \eta\nabla F(\btheta_1).
$$
%which is exactly what we have discussed in Section~\ref{section:regularization-extention-general}, the regularization term.

\index{Stochastic gradient descent}
\index{Gradient descent}
\index{Matrix inverse}
\index{LU decomposition}
\subsection{ALS via Gradient Descent}\label{section:als-gradie-descent}
In Algorithm~\ref{alg:als}, \ref{alg:als-regularizer}, and \ref{alg:als-regularizer-missing-entries}, we reduce the loss of the low-rank approximation or the alternating least squares (ALS) problem through the inversion of matrices (e.g., using LU decomposition \citep{lu2021numerical}). 
The reality, however, is frequently far from straightforward, particularly in the big data era of today. As data volumes explode, the size of the inversion matrix will grow at a pace proportional to the cube of the number of samples,  which poses a great challenge to the storage and computational resources.
On the other hand, this leads to the creation of an ongoing development of the gradient-based optimization technique.
The \textit{gradient descent (GD)} method and its variant, the \textit{stochastic gradient descent (SGD)} method, are among them the simplest, fastest, and most efficient methods. Convex loss function optimization problems are frequently solved using this type of approach. We now go into more details about its principle in the ALS context.

\subsection*{Gradient Descent}
In Equation~\eqref{equation:als-ori-all-wz}, we derive the column-by-column update rules for ALS directly from the full matrix approach outlined in Equation~\eqref{equation:als-regular-final-all} (with regularization taken into account). 
To understand the underlying concept, consider the loss function with regularization, as given by Equation~\eqref{equation:als-regularion-full-matrix}.
When minimizing the  loss in \eqref{equation:als-regularion-full-matrix} with respect to $\bz_p$, we can break down the loss as follows:
\begin{equation}\label{als:gradient-regularization-zn}
\footnotesize
\begin{aligned}
L(\bz_p)  &=\frac{1}{2}\normf{\bW\bZ-\bX}^2 +\frac{1}{2}\lambda_w \normf{\bW}^2 + \frac{1}{2}\lambda_z \normf{\bZ}^2
= \frac{1}{2}\normtwo{\bW\bz_p-\bx_p}^2 + \frac{1}{2}\lambda_z \normtwo{\bz_p}^2 + C_{z_p},
\end{aligned}
\end{equation}
where $C_{z_p}$ is a constant with respect to $\bz_p$, and $\bZ=[\bz_1, \bz_2, \ldots, \bz_P]$ and $\bX=[\bx_1,\bx_2, \ldots, \bx_P]$ represent the column partitions of $\bZ$ and $\bX$, respectively. 
The gradient and the root are given, respectively, by 
$$
\begin{aligned}
\nabla_{\bz_p} L(\bz_p) = \bW^\top\bW\bz_p - \bW^\top\bx_p + \lambda_z\bz_p
\quad\implies \quad
\bz_p = (\bW^\top\bW+ \lambda_z\bI)^{-1} \bW^\top \bx_p, \,\,  \forall\,p.
\end{aligned}
$$
This solution corresponds to the first update rule in the column-wise update in Equation~\eqref{equation:als-ori-all-wz}.
Similarly, when minimizing the loss with respect to $\bw_n$, we have:
\begin{equation}\label{als:gradient-regularization-wd}
\footnotesize
\begin{aligned}
L(\bw_n )  
&
=\frac{1}{2}\normf{\bZ^\top\bW-\bX^\top}^2 +\frac{1}{2}\lambda_w \normf{\bW^\top}^2 + \frac{1}{2}\lambda_z \normf{\bZ}^2
= \frac{1}{2}\normtwo{\bZ^\top\bw_n-\bb_p}^2 + \frac{1}{2}\lambda_w \normtwo{\bw_n}^2 + C_{w_n},
\end{aligned}
\end{equation}
where $C_{w_n}$ is a constant with respect to $\bw_n$, and $\bW^\top=[\bw_1, \bw_2, \ldots, \bw_N]$ and $\bX^\top=[\bb_1,\bb_2, \ldots,$ $\bb_N]$ represent the column partitions of $\bW^\top$ and $\bX^\top$, respectively. 
Analogously, taking the gradient with respect to $\bw_n$, it follows that
$$
\begin{aligned}
\nabla_{\bw_n} L(\bw_n) = \bZ\bZ^\top\bw_n - \bZ\bb_p + \lambda_w\bw_n
\quad\implies\quad
\bw_n = (\bZ\bZ^\top+\lambda_w\bI)^{-1}\bZ\bb_n, \,\, \forall \, n.
\end{aligned}
$$
This solution corresponds to the second update rule in the column-wise update in Equation~\eqref{equation:als-ori-all-wz}:

Now suppose we express the iteration number ($t=1,2,\ldots$)  as the superscript, and we want to find the updates $\{\bz^{(t+1)}_p, \bw^{(t+1)}_n\}$  in the $(t+1)$-th iteration  base on $\{\bZ^{(t)}, \bW^{(t)}\}$  in the $t$-th iteration:
$$
\left.
\begin{aligned}
\bz^{(t+1)}_p    &\leftarrow \mathop{\arg \min}_{\bz_p^{(t)}} L(\bz_p^{(t)})
\qquad\text{and}\qquad
\bw_n^{(t+1)}    \leftarrow \mathop{\arg\min}_{\bw_n^{(t)}} L(\bw_n^{(t)}).
\end{aligned}
\right.
$$
For simplicity, we will only derive for $\bz^{(t+1)}_p    \leftarrow \mathop{\arg \min}_{\bz_p^{(t)}} L(\bz_p^{(t)})$, and the derivation for the update on $\bw_n^{(t+1)}$ will follow a similar process.
With this insight, the gradient descent update for the ALS algorithms is provided in Algorithm~\ref{alg:als-regularizer-missing-stochas-gradient}.

It's noteworthy that the ALS without GD (Algorithm~\ref{alg:als-regularizer})  lacks explicit parameters like step size. 
This characteristic can be both advantageous and disadvantageous. 
On one hand, it absolves the  user from the time-consuming task of fine-tuning parameters, making the method more accessible and less demanding. 
On the other hand, this absence of adjustable parameters also restricts the user's control to directly influence the progression of the algorithm, leaving the convergence of ALS entirely contingent upon the inherent structure of the optimization problem at hand.

In practical applications, it is customary to alternate between the pure ALS iterations outlined in Algorithm~\ref{alg:als-regularizer} and the modified, gradient-descent variants  mentioned in this section. These descent adaptations offer the user a degree of control through a tunable step length parameter, allowing for a more customized approach to the optimization process.

\begin{algorithm}[h] 
\caption{Alternating Least Squares with Full Entries and Gradient Descent}
\label{alg:als-regularizer-missing-stochas-gradient}
\begin{algorithmic}[1] 
\Require Matrix $\bX\in \real^{N\times P}$;
\State Initialize $\bW\in \real^{N\times K}$, $\bZ\in \real^{K\times P}$ \textcolor{mylightbluetext}{randomly without condition on the rank and the relationship between $N, P, K$}; 
\State Choose a stoping criterion on the approximation error $\delta$;
\State Choose regularization parameters $\lambda_w, \lambda_z$, and step sizes $\eta_w, \eta_z$;
\State Choose the maximum number of iterations $C$;
\State $iter=0$; \Comment{Count for the number of iterations}
\While{$\normf{\bX- (\bW\bZ)}^2>\delta $ and $iter<C$}
\State $iter=iter+1$; 
\For{$p=1,2,\ldots, P$}
\State $\bz^{(t+1)}_p \leftarrow\bz^{(t)}_p - \eta_z {\nabla L(\bz^{(t)}_p )}$; %\big/{\big\Vert{\nabla L(\bz^{(t)}_p )}\big\Vert_2}$;
\Comment{$p$-th column of $\bZ$}
\EndFor

\For{$n=1,2,\ldots, N$}
\State $\bw^{(t+1)}_n  \leftarrow \bw^{(t)}_n - \eta_w {\nabla L(\bw^{(t)}_n )}$; %\big/{\big\Vert{\nabla L(\bw^{(t)}_n )}\big\Vert_2}$;
\Comment{$n$-th column of $\bW^\top$}
\EndFor
\EndWhile
\State Output $\bW^\top=[\bw_1, \bw_2, \ldots, \bw_N],\bZ=[\bz_1, \bz_2, \ldots, \bz_P]$;
\end{algorithmic} 
\end{algorithm}

\index{Stochastic gradient descent}
\index{Stochastic coordinate descent}
\subsection*{Stochastic Gradient Descent}
The gradient descent method is a valuable optimization algorithm; however, it exhibits certain limitations in practical applications. 
To comprehend the issues associated with the gradient descent method, we consider the mean squared error (MSE) derived from  \eqref{equation:als-per-example-loss}:
\begin{equation}\label{equation:als-per-example-loss_mse}
\frac{1}{NP}\mathop{\min}_{\bW,\bZ}  \sum_{p=1}^P \sum_{n=1}^{N} \left(x_{np} - \bw_n^\top\bz_p\right)^2.
\end{equation}
The MSE requires calculating the residual $e_{np} \triangleq (x_{np} - \bw_n^\top\bz_p)^2$ for each observed entry $x_{np}$, representing the squared difference between predicted and actual values.  The total sum of residual squares is denoted by $e = \sum_{n,p=1}^{NP}e_{np}$.
In cases with a substantial number of training entries (i.e., large $NP$), the entire computation process becomes notably slow. 
Additionally, the gradients from different input samples may cancel out, resulting in small changes in the final update.
As mentioned previously, researchers have enhanced the gradient descent method with the \textit{stochastic gradient descent (SGD)} method  to address these challenges. 
In the SGD algorithm, instead of calculating the full gradient of the objective function with respect to the parameters  across all samples in the data set, which can be computationally expensive, the algorithm takes a more efficient approach. It randomly chooses one sample and calculates the gradient of the objective function with respect to the parameters using only  this single sample. 
This gradient estimate is then used to update the parameters in the direction that minimizes the objective function. 
By using a single sample at each iteration, the SGD algorithm provides a fast and often sufficient approximation of the full gradient, making it particularly useful for large data sets.

In particular, we consider again the per-example loss:
$$
L(\bW,\bZ)= \frac{1}{2} \sum_{p=1}^P \sum_{n=1}^{N} \left(x_{np} - \bw_n^\top\bz_p\right)^2 +\frac{1}{2} \lambda_w\sum_{n=1}^{N}\normtwo{\bw_n}^2 +\frac{1}{2}\lambda_z\sum_{p=1}^{P}\normtwo{\bz_p}^2.
$$
As we iteratively reduce the  loss term $l(\bw_n, \bz_p)=\frac{1}{2}\left(x_{np} - \bw_n^\top\bz_p\right)^2+\frac{1}{2} \lambda_w\normtwo{\bw_n}^2 +\frac{1}{2}\lambda_z\normtwo{\bz_p}^2$ for all $n\in \{1,2,\ldots, N\}, p\in\{1,2,\ldots,P\}$ (referred to as the per-example loss term), the overall loss $L(\bW,\bZ)$ decreases accordingly.
This process is also known as  \textit{stochastic coordinate descent}. The gradients  with respect to $\bw_n$ and $\bz_p$, and their roots are given, respectively, by 
$$
\left\{
\begin{aligned}
\nabla_{\bz_p} l(\bz_p) &= \bw_n\bw_n^\top \bz_p +\lambda_z\bz_p  -x_{np} \bw_n 
&\implies&\,\, \bz_p= x_{np}(\bw_n\bw_n^\top+\lambda_z\bI)^{-1}\bw_n;\\
\nabla_{\bw_n} l(\bw_n) &= \bz_p\bz_p^\top\bw_n +\lambda_w\bw_n - x_{np}\bz_p &\implies&\,\, \bw_n= x_{np}(\bz_p\bz_p^\top+\lambda_w\bI)^{-1}\bz_p.
\end{aligned}
\right.
$$
Alternatively, the update can be performed using gradient descent. Since we update based on the per-example loss, this approach is also known as the \textit{stochastic gradient descent (SGD)}:
$$
\left.
\begin{aligned}
\bz_p&\leftarrow \bz_p - \eta_z {\nabla_{\bz_p} l(\bz_p)} %{\normtwo{\nabla_{\bz_p} l(\bz_p)}}
\qquad \text{and}\qquad 
\bw_n\leftarrow \bw_n - \eta_w {\nabla_{\bw_n} l(\bw_n)} %{\normtwo{\nabla_{\bw_n} l(\bw_n)}}.
\end{aligned}
\right.
$$
The stochastic gradient descent update for ALS is formulated in Algorithm~\ref{alg:als-regularizer-missing-stochas-gradient-realstoch}. 
It is possible that the gradient descent or stochastic gradient descent algorithm may fail to converge. In such cases, it is appropriate to re-run the algorithm using a smaller learning rate.
And in practice, the values of $n$ and $p$ in the algorithm can be randomly generated, which is why the method is termed ``stochastic." ~\footnote{When we iteratively choose the values of $n$ and $p$ from $\{1,2,\ldots, N\}$ and $\{1,2,\ldots, P\}$ in a deterministic cyclic order, respectively, the stochastic method can be referred to as  ``\textit{incremental gradient descent}."}

\begin{algorithm}[h] 
\caption{Alternating Least Squares with Full Entries and SGD}
\label{alg:als-regularizer-missing-stochas-gradient-realstoch}
\begin{algorithmic}[1] 
\Require  Matrix $\bX\in \real^{N\times P}$;
\State Initialize $\bW\in \real^{N\times K}$, $\bZ\in \real^{K\times P}$ \textcolor{mylightbluetext}{randomly without condition on the rank and the relationship between $N, P, K$}; 
\State Choose a stoping criterion on the approximation error $\delta$;
\State Choose regularization parameters $\lambda_w, \lambda_z$, and step size $\eta_w, \eta_z$;
\State Choose the maximum number of iterations $C$;
\State $iter=0$; \Comment{Count for the number of iterations}
\While{$\normf{  \bX- (\bW\bZ)}^2>\delta $ and $iter<C$}
\State $iter=iter+1$; 
\For{$p=1,2,\ldots, P$}
\For{$n=1,2,\ldots, N$} \Comment{in practice, $n,p$ can be randomly produced}
\State $\bz_p\leftarrow \bz_p - \eta_z {\nabla l(\bz_p)}$; %/{\normtwo{\nabla l(\bz_p)}}$;
\Comment{$p$-th column of $\bZ$}
\State $\bw_n\leftarrow \bw_n - \eta_w {\nabla l(\bw_n)}$; %/{\normtwo{\nabla l(\bw_n)}}$;
\Comment{$n$-th column of $\bW^\top$}
\EndFor
\EndFor

\EndWhile
\State Output $\bW^\top=[\bw_1, \bw_2, \ldots, \bw_M],\bZ=[\bz_1, \bz_2, \ldots, \bz_P]$;
\end{algorithmic} 
\end{algorithm}

\begin{exercise}
Following the missing entry update in ALS (Section~\ref{section:alt-columb-by-column}), derive a mini-batch SGD version for the ALS problems. Note that the mini-batch SGD is a balance between the GD (where the computation can be extensive) and the strict SGD (where the interaction between different entries can be mitigated, which is important in the Netflix context) algorithms.
\end{exercise}

%\subsection*{Newton's Method$^*$}\label{section:als_quasi_newton}
%The ALS algorithms, whether using GD or SGD, fall under the category of first-order optimization algorithms.
%However, we can also apply  Newton's method, to derive a second-order update method:
%\begin{equation}
%\begin{aligned}
%\bZ &\leftarrow \bZ - \big(\nabla_{\bZ}^2L(\bZ|\bW)\big)^{-1}  \nabla_{\bZ} L(\bZ|\bW) = \bZ - (\bW^\top\bW)^{-1} \bW^\top (\bW\bZ-\bA);\\
%\bW &\leftarrow \bW -   \nabla_{\bW} L(\bW|\bZ) \big(\nabla_{\bW}^2L(\bW|\bZ)\big)^{-1} = \bW -  (\bW\bZ-\bA)\bZ^\top (\bZ\bZ^\top)^{-1}.
%\end{aligned}
%\end{equation}
%In order to prevent the Hessian matrix from being singular or having a large condition number, one can still incorporate the $\ell_2$ regularization:
%\begin{equation}
%\begin{aligned}
%\bZ &\leftarrow \bZ - \big(\nabla_{\bZ}^2L(\bZ|\bW)\big)^{-1}  \nabla_{\bZ} L(\bZ|\bW) = \bZ - (\bW^\top\bW+\lambda\bI)^{-1} \bW^\top (\bW\bZ-\bA);\\
%\bW &\leftarrow \bW -   \nabla_{\bW} L(\bW|\bZ) \big(\nabla_{\bW}^2L(\bW|\bZ)\big)^{-1} = \bW -  (\bW\bZ-\bA)\bZ^\top (\bZ\bZ^\top+\lambda\bI)^{-1}.
%\end{aligned}
%\end{equation}

\index{Cholesky decomposition}
\index{Positive definiteness}
\section{LS via Cholesky Decomposition}\label{section:ls_cholesky}

As mentioned in Chapter~\ref{chapter:ls_approx}, the LS problem can be solved using the Cholesky decomposition. We begin by providing ways to compute the Cholesky decomposition of a positive definite matrix.

\paragrapharrow{Computing  Cholesky decomposition recursively.}
To compute the Cholesky decomposition, we start by writing out the equality $\bA=\bR^\top\bR$, where $\bR$ is upper triangular:
$$
\footnotesize
\setlength{\arraycolsep}{2pt}
\begin{aligned}
\bA=
\begin{bmatrix}
a_{11} & \bA_{1,2:n} \\
\bA_{2:n,1} & \bA_{2:n,2:n}
\end{bmatrix}
&=
\begin{bmatrix}
r_{11} & 0 \\
\bR_{1,2:n}^\top & \bR_{2:n,2:n}^\top
\end{bmatrix}
\begin{bmatrix}
r_{11} & \bR_{1,2:n} \\
0 & \bR_{2:n,2:n}
\end{bmatrix}
=
\begin{bmatrix}
r_{11}^2 & r_{11}\bR_{1,2:n} \\
r_{11}\bR_{1,2:n}^\top & \bR_{1,2:n}^\top\bR_{1,2:n} + \bR_{2:n,2:n}^\top\bR_{2:n,2:n}
\end{bmatrix}.
\end{aligned}
$$
Since the diagonals of $\bR$ are positive when $\bA$ is positive definite (Theorem~\ref{theorem:cholesky-factor-exist}), 
this allows us to determine the first row of $\bR$ by 
$$
r_{11} = \sqrt{a_{11}}
\qquad \text{and}\qquad 
\bR_{1,2:n} = \frac{1}{r_{11}}\bA_{1,2:n}.
$$
Let $\bA_2\triangleq \bR_{2:n,2:n}^\top\bR_{2:n,2:n}$. The equality $\bA_{2:n,2:n} = \bR_{1,2:n}^\top\bR_{1,2:n} + \bR_{2:n,2:n}^\top\bR_{2:n,2:n}$ and the symmetry of $\bA$ indicate
$$
\begin{aligned}
\bA_2=\bR_{2:n,2:n}^\top\bR_{2:n,2:n} &= \bA_{2:n,2:n} - \bR_{1,2:n}^\top\bR_{1,2:n} 
%= \bA_{2:n,2:n} - \frac{1}{a_{11}} \bA_{1,2:n}^\top\bA_{1,2:n} \\
= \bA_{2:n,2:n} - \frac{1}{a_{11}} \bA_{2:n,1}\bA_{1,2:n},
\end{aligned}
$$
where $\bA_2$ is known as the \textit{Schur complement} of $a_{11}$ in $\bA$ and has a  size of $(n-1)\times (n-1)$. To obtain $\bR_{2:n,2:n}$, we must compute the Cholesky decomposition of the matrix $\bA_2$ of shape $(n-1)\times (n-1)$. 
This implies a recursive algorithm to computing the Cholesky decomposition of a PD matrix $\bA$, and the procedure is outlined in Algorithm~\ref{alg:compute-choklesky}.

\begin{algorithm}[H] 
\caption{Cholesky Decomposition via Recursive Algorithm} 
\label{alg:compute-choklesky} 
\begin{algorithmic}[1] 
\Require 
Positive definite matrix $\bA$ with size $n\times n$; 
\State Calculate first row of $\bR$ by $r_{11} \leftarrow \sqrt{a_{11}}, \bR_{1,2:n} \leftarrow \frac{1}{r_{11}}\bA_{1,2:n}$; 
%\Comment{$n$ flops}
\State Compute the Cholesky decomposition of the $(n-1)\times (n-1)$ matrix
$$
\bA_2\leftarrow\bR_{2:n,2:n}^\top\bR_{2:n,2:n}=\bA_{2:n,2:n} - \frac{1}{a_{11}} \bA_{2:n,1}\bA_{1,2:n};
$$
%\Comment{$n^2-n$ flops}

\end{algorithmic} 
\end{algorithm}

%Further, this process can be used to determine if a matrix is positive definite or not. If we try to factor a nonpositive definite matrix, at some point, we will encounter a nonpositive element in entry (1,1) of one of the matrices $\bA, \bA_2, \bA_3, \ldots$.

\begin{theoremHigh}[Algorithm complexity: Cholesky recursively \citep{lu2021numerical}]\label{theorem:cholesky-complexity}
Algorithm~\ref{alg:compute-choklesky} requires $\sim(1/3)n^3$ flops to compute the Cholesky decomposition of an $n\times n$ positive definite matrix.
\end{theoremHigh}

%\begin{proof}[of Theorem~\ref{theorem:cholesky-complexity}]
%Step 1 involves 1 square root and $(n-1)$ divisions, which take $n$ flops totally.  
%
%For step 2, note that $\frac{1}{a_{11}} \bA_{2:n,1}\bA_{1,2:n} = \left(\frac{1}{\sqrt{a_{11}}} \bA_{2:n,1}\right)\left(\frac{1}{\sqrt{a_{11}}}\bA_{1,2:n}\right) = \bR_{1,2:n}^\top \bR_{1,2:n}$. If we calculate the complexity directly from the equation in step 2, we will get the same complexity as the LU decomposition. But the symmetry of $\bR_{1,2:n}^\top \bR_{1,2:n}$ can be adopted, the complexity of $\bR_{1,2:n}^\top \bR_{1,2:n}$ reduces from $(n-1)\times(n-1)$ multiplications to $1+2+\ldots+(n-1)=\frac{n^2-n}{2}$ multiplications, which is almost half of the original complexity. The cost of matrix division reduces from $(n-1)\times(n-1)$ to $1+2+\ldots+(n-1)=\frac{n^2-n}{2}$ as well. So the total cost for step 2 is $n^2-n$ flops.
%
%Let $f(n) = n^2-n + n = n^2$, the total complexity is 
%$$
%\mathrm{cost} = f(n)+f(n-1)+\ldots +f(1).
%$$
%Simple calculations  show that  the total complexity for all the recursive steps is $\frac{2n^3+3n^2+n}{6}$ flops, which is $(1/3)n^3$ flops if we keep only the leading term.
%\end{proof}

The Cholesky decomposition computation mentioned above has an important application in testing the positive definiteness of a symmetric matrix. To perform the test, one can apply the algorithm mentioned above and declare the matrix as positive definite if the algorithm completes without encountering any negative or zero pivots (as described in step 1 above). Otherwise, if the algorithm encounters such pivots, the matrix is deemed not positive definite.

To end up this section, we provide the full pseudo code for Algorithm~\ref{alg:compute-choklesky} as shown in Algorithm~\ref{alg:compute-choklesky11} (compare the two algorithms).  
\begin{algorithm}[H] 
\caption{Cholesky Decomposition via Recursive Algorithm: Full Pseudo Code} 
\label{alg:compute-choklesky11} 
\begin{algorithmic}[1] 
\Require 
Positive definite matrix $\bA$ with size $n\times n$; 
\For{$k=1$ to $n$} \Comment{compute the $k$-th row of $\bR$}
\State $r_{kk} \leftarrow \sqrt{a_{kk}}$; \Comment{first element of $k$-th row}
\State $\bR_{k,k+1:n} \leftarrow \frac{1}{r_{kk}} \bA_{k,k+1:n}$; \Comment{the rest elements of $k$-th row}
\State $\bA_{k+1:n,k+1:n} \leftarrow \bA_{k+1:n,k+1:n} - \bR_{k,k+1:n}^\top\bR_{k,k+1:n}$; %\Comment{$2(1+2+\ldots+(n-k))$ flops}
\EndFor
\State Output $\bA=\bR^\top\bR$.
\end{algorithmic} 
\end{algorithm}

\paragrapharrow{An alternative perspective of the recursive algorithm.}
Since $\bL\triangleq \bR^\top$ is lower triangular. The lower triangular factor $\bL$ can be computed as a product of a series of lower triangular matrices. 
To see this, we have
$$
\bA=
\begin{aligned}
\begin{bmatrix}
a_{11} & \bA_{1,2:n} \\
\bA_{2:n,1} & \bA_{2:n,2:n}
\end{bmatrix} 
=
\begin{bmatrix}l_{11} & \bzero \\ 
\bL_{21} & \bL_{22} 
\end{bmatrix}
\begin{bmatrix}
l_{11} & \bL_{21}^\top \\ 
\bzero & \bL_{22}^\top 
\end{bmatrix} 
\end{aligned} 
\triangleq \bL\bL^\top.
$$
Then we still have 
$$
\begin{bmatrix}
a_{11} & \bA_{1,2:n} \\
\bA_{2:n,1} & \bA_{2:n,2:n}
\end{bmatrix}  
= 
\begin{bmatrix}
l_{11}^2 & l_{11}\bL_{21}^\top \\ 
l_{11}\bL_{21} & \bL_{21}\bL_{21}^\top+\bL_{22}\bL_{22}^\top 
\end{bmatrix} 
\implies
\begin{cases}
l_{11} &= \sqrt{a_{11}}; \\ 
\bL_{21} &= \frac{1}{l_{11}}\bA_{2:n,1}; \\ 
\bL_{22}\bL_{22}^\top &= \bA_{2:n,2:n} - \bL_{21}\bL_{21}^\top .
\end{cases}
$$
The second perspective involves constructing $n+1$ set of  $n \times n$ matrices: $\bA^{(1)},\bA^{(2)}, \ldots, \bA^{(n+1)}$, where $\bA^{(1)}\triangleq\bA$, and we want to obtain $\bA^{(n+1)}=\bI$ via the relation:
\begin{equation}\label{equation:choes_recur_secon}
\bA^{(i)} = \bL^{(i)}\bA^{(i+1)}\bL^{(i)^\top}, \ \forall\, i \in\{1,2,\ldots,n\}.
\end{equation}
If these $\bL^{(i)}, \ \forall\, i$ are lower triangular, then we obtain the Cholesky decomposition by 
$$
\bA= 
(\bL^{(1)} \bL^{(2)}\ldots \bL^{(n)}) (\bL^{(1)} \bL^{(2)}\ldots \bL^{(n)})^\top
\triangleq \bL\bL^\top.
$$
This is indeed the case. To see this, we can construct 
$$
\bA^{(i)} \triangleq
\begin{bmatrix} 
\bI_{i-1} & 0 & \bzero \\ 
0 & a_{ii} & \bb_i^\top \\ 
\bzero & \bb_i & \bB^{(i)} 
\end{bmatrix} 
\qquad\text{and}\qquad
\bL^{(i)} = 
\begin{bmatrix} 
\bI_{i-1} & 0 & \bzero \\ 
0 & \sqrt{a_{ii}} & \bzero \\ 
\bzero & \frac{1}{\sqrt{a_{ii}}}\bb_i & \bI_{n-i} 
\end{bmatrix},
$$
satisfying  $\bA^{(i)} = \bL^{(i)}\bA^{(i+1)}(\bL^{(i)})^\top$:
$$
\begin{aligned} \bA^{(i+1)} 
&= \begin{bmatrix} 
\bI_{i-1} & 0 & \bzero \\ 
0 & 1 & \bzero \\ 
\bzero & \bzero & \bB^{(i)}-\frac{1}{a_{ii}}\bb_i\bb_i^\top 
\end{bmatrix} 
\triangleq 
\begin{bmatrix} 
\bI_i & 0 & \bzero \\ 
0 & a_{i+1, i+1} & \bb_{i+1}^\top \\ 
\bzero & \bb_{i+1} & \bB^{(i+1)} 
\end{bmatrix}.
\end{aligned}
$$
Therefore, $\bA$ can be decomposed as a set of lower triangular matrices in \eqref{equation:choes_recur_secon}.
Using the result in Exercise~\ref{exercise:choe_recur_sec} can show that the algorithm for this perspective is equivalent to Algorithm~\ref{alg:compute-choklesky11}.

\begin{exercise}\label{exercise:choe_recur_sec}
Verify that $\bL^{(i)}_{i:,i} = \bL_{i:,i}, i = 1,2,\ldots, n$.
\end{exercise}

\paragrapharrow{Computing Cholesky decomposition element-wise.}

It is also common to compute the Cholesky decomposition using element-wise equations derived directly from solving the matrix equation $\bA=\bR^\top\bR$. Observing that the $(i,j)$-th entry of $\bA$ is given by $a_{ij} = \bR_{:,i}^\top \bR_{:,j} = \sum_{k=1}^{i} r_{ki}r_{kj}$ if $i<j$. This further implies the  following recurrence relation: if $i<j$, we have
$$
\begin{aligned}
a_{ij} &= \bR_{:,i}^\top \bR_{:,j} = \sum_{k=1}^{i} r_{ki}r_{kj} 
= \sum_{k=1}^{i-1} r_{ki}r_{kj} + r_{ii}r_{ij}
\implies
r_{ij} = (a_{ij} - \sum_{k=1}^{i-1} r_{ki}r_{kj})/r_{ii},
\gap 
\text{if }i<j.
\end{aligned}
$$
For the diagonal entries ($i=j$), we have:
$$
\begin{aligned}
a_{jj} &= \sum_{k=1}^{j} r_{kj}^2=\sum_{k=1}^{j-1} r_{kj}^2 + r_{jj}^2
&\implies
r_{jj} = \sqrt{a_{jj} - \sum_{k=1}^{j-1} r_{kj}^2}.
\end{aligned}
$$
If we equate the elements of $\bR$ by taking a column at a time and start with $r_{11} = \sqrt{a_{11}}$, the element-level algorithm is formulated in Algorithm~\ref{alg:compute-choklesky-element-level}.

\begin{algorithm}[H] 
\caption{Cholesky Decomposition Element-Wise: $\bA=\bR^\top\bR$} 
\label{alg:compute-choklesky-element-level} 
\begin{algorithmic}[1] 
\Require 
Positive definite matrix $\bA$ with size $n\times n$; 
\State Calculate first element of $\bR$ by $r_{11} \leftarrow \sqrt{a_{11}}$; 
\For{$j=1$ to $n$} \Comment{Compute the $j$-th column of $\bR$}
\For{$i=1$ to $j-1$} 
\State $r_{ij} \leftarrow (a_{ij} - \sum_{k=1}^{i-1} r_{ki}r_{kj})/r_{ii}$, since $i<j$;
\EndFor
\State $r_{jj} \leftarrow \sqrt{a_{jj}- \sum_{k=1}^{j-1}r_{kj}^2}$;
\EndFor
\State Output $\bA=\bR^\top\bR$.
\end{algorithmic} 
\end{algorithm}
On the other hand, Algorithm~\ref{alg:compute-choklesky-element-level}  can be modified to compute the Cholesky decomposition in the form $\bA=\bL\bD\bL^\top$, where $\bL$ is unit lower triangular and $\bD$ is diagonal, as outlined in Algorithm~\ref{alg:compute-choklesky-_ldl}, whose Step 3 and Step 5 are derived from (since $l_{ii}=1, \forall\, i\in\{1,2,\ldots,n\}$):
$$
\begin{aligned}
a_{jj}&=\sum_{k=1}^{j-1}d_{kk} l_{jk}^2 + d_{jj};\\
a_{ij}&= d_{jj} l_{ij}+ \sum_{k=1}^{j-1} d_{kk} l_{ik}l_{jk}, \gap \text{if }i>j.
\end{aligned}
$$
\begin{exercise}
Derive the complexity of Algorithm~\ref{alg:compute-choklesky-_ldl}.
\end{exercise}
This form of Cholesky decomposition is useful for determining the condition number of a PD matrix \citep{lu2021numerical}. 
In essence, the condition number of a function measures the sensitivity of the output value to small changes in the input; a smaller condition number indicates better numerical stability. 
For positive definite linear systems, the condition number is defined as the ratio of the largest eigenvalue to the smallest eigenvalue.
The condition number of a positive definite matrix is lower bounded by the diagonal matrix in the Cholesky decomposition (see Problem~\ref{problem:cond_pd}):
\begin{equation}\label{equation:cond_pd_ineq}
\cond(\bA) \geq \cond(\bD).
\end{equation}
This can be proven by showing that $\lambda_{\max}\geq d_{\max}$ and $\lambda_{\min}\leq d_{\min}$, where $\lambda_{\max}$ and $\lambda_{\min}$ are the largest and smallest eigenvalue of $\bA$, and $d_{\max}$ and $d_{\min}$ are the largest and smallest diagonals of $\bD$.
Therefore, this form of the Cholesky decomposition can be utilized  to modify Newton's method; see \citep{lu2025practical}.

\begin{algorithm}[h] 
\caption{Cholesky Decomposition Element-Wise: $\bA=\bL\bD\bL^\top$}  
\label{alg:compute-choklesky-_ldl} 
\begin{algorithmic}[1] 
\Require 
Positive definite matrix $\bA$ with size $n\times n$; 
\For{$j=1$ to $n$} \Comment{Compute the $j$-th column of $\bL$}
\State $l_{jj}\leftarrow1$;
\State $c_{jj}\leftarrow a_{jj}-\sum_{k=1}^{j-1}d_{kk} l_{jk}^2$;
\State $d_{jj}\leftarrow c_{jj}$
\For{$i=j+1$ to $n$} 
\State $c_{ij}\leftarrow a_{ij}-\sum_{k=1}^{j-1}d_{kk} l_{ik}l_{jk}$, since $i>j$;
\State $l_{ij}\leftarrow \frac{c_{ij}}{d_{jj}}$;
\EndFor
\EndFor
\State Output $\bA=\bL\bD\bL^\top$, where $\bD=\diag(d_{11}, d_{22},\ldots,d_{nn})$.
\end{algorithmic} 
\end{algorithm}

\subsection{Full Rank Case}
The classical method for solving a linear least squares problem $\min_{\bbeta} \normtwo{\bX\bbeta - \by}$, $\bX \in \real^{n \times p}$, is to form and solve the symmetric normal equation $\bX^\top \bX\bbeta = \bX^\top \by$. If $\rank(\bX) = p$, then $\bbeta \neq \bzero$ implies that $\bX\bbeta \neq \bzero$. Hence
\begin{equation}
\bbeta^\top \bX^\top \bX\bbeta > 0, \quad \forall\, \bbeta \in \real^p, \quad \bbeta \neq \bzero,
\end{equation}
and $\bX^\top \bX$ is positive definite. Conversely, a symmetric positive definite matrix is nonsingular. If it were singular, there would be a vector $\bbeta$ such that $\bX\bbeta = \bzero$. But then $\bbeta^\top \bX\bbeta = 0$, which is a contradiction.

\index{Normal equation}
\paragrapharrow{Normal equation of the first kind.}

Substituting the Cholesky factorization $ \bX^\top \bX = \bR^\top \bR $ into the normal equation gives $ \bR^\top \bR \bbeta = \balpha $, where $ \balpha \triangleq \bX^\top \by $. Hence, the solution is obtained by solving two triangular systems:
\begin{equation}
\bR^\top \bu = \balpha, \qquad \bR \bbeta = \bu.
\end{equation}
This method is easy to implement and often faster than other direct solution methods. It works well unless $ \bX $ is ill-conditioned.

It is often preferable to work with the Cholesky factorization of the cross-product of the augmented matrix $[\bX ,\by ]$:
\begin{equation}\label{equation:corss_prod_eq}
\begin{bmatrix} 
\bX^\top \\ 
\by^\top 
\end{bmatrix} 
\begin{bmatrix} 
\bX & \by 
\end{bmatrix} 
= 
\begin{bmatrix} 
\bX^\top \bX & \bX^\top \by \\ 
\by^\top \bX & \by^\top \by
\end{bmatrix},
\end{equation}
when solving a least squares problem. If $\rank(\bX) = p$, then the Cholesky factor of the cross-product \eqref{equation:corss_prod_eq},
\begin{equation}
\bS = \begin{bmatrix} 
\bR & \bv \\ 
\bzero & \rho 
\end{bmatrix},
\end{equation}
exists, where we may have $\rho = 0$. Forming $\bS^\top \bS$ shows that
$$
\bX^\top \bX = \bR^\top \bR, \qquad \bR^\top \bv = \bX^\top \by, \qquad \by^\top \by = \bv^\top \bv + \rho^2.
$$
Hence, $\bR$ is the Cholesky factor of $\bX^\top \bX$, and the least squares solution is obtained from $\bR\bbeta = \bv$. Since $\be = \by - \bX\bbeta$ is orthogonal to $\bX\bbeta$, we have
$$
\normtwo{\bX\bbeta}^2 = (\be + \bX\bbeta)^\top \bX\bbeta = \by^\top \bX \bbeta = \by^\top \bX \bR^{-1} \bR^{-\top} \bX^\top \by = \bv^\top \bv,
$$
and hence $\normtwo{\be}^2 = \rho^2 = \by^\top \by - \bv^\top \bv$ and $\normtwo{\by - \bX\bbeta} = \rho$.

\paragrapharrow{Cholesky QR factorization.}
On the other hand, let $\bX \in \real^{n\times p}$ have full column rank, and let $\bX^\top \bX = \bR^\top \bR$ be its Cholesky factorization. 
Define $\bQ_1 \triangleq \bX \bR^{-1} \in\real^{n\times p}$. Then, 
\begin{equation}\label{equation:qr_cho_dec}
	\bX = \bQ_1 \bR 
	\qquad\text{and}\qquad 
	\bQ_1^\top \bQ_1 = \bI_p
\end{equation}
is the \textit{Cholesky QR factorization} of $\bX$. The semi-orthogonal factor $\bQ_1$ can be obtained as the unique solution of the lower triangular matrix equation $\bR^\top \bQ_1^\top = \bX^\top$ by forward substitution. In this setting, the normal equation simplifies to $\bR^\top \bQ_1^\top \bQ_1 \bR\bbeta = \bR^\top \bR\bbeta = \bR^\top \bQ_1^\top \by$ or
$ \bR\bbeta = \bQ_1^\top \by. $

In the real case, the computational cost of this Cholesky QR algorithm is $\sim 2np^2 + p^3/3$ flops. 
More accurate methods for computing the QR factorization \eqref{equation:qr_cho_dec} directly from $\bX$ are described in Section~\ref{section:ls_qr_gen}.

\paragrapharrow{Normal equation of the second kind.}
For a consistent underdetermined linear system $ \bX\bbeta = \by $, the solution to the least-norm problem $\min \normtwo{\bbeta}$ subject to $ \bX\bbeta = \by $ satisfies the normal equation of the second kind in \eqref{equation:consis_minimunorm}:
$$ 
\bbeta = \bX^\top \bgamma
\qquad\text{and}\qquad 
\bX \bX^\top \bgamma = \by. 
$$
If $ \bX $ has full row rank, then $ \bX \bX^\top $ is symmetric positive definite, and the Cholesky factorization $ \bX \bX^\top = \bR^\top \bR $ exists. Then $ \bgamma $ is obtained by solving
\begin{equation}
\bR^\top \bw = \by, \qquad \bR \bgamma = \bw.
\end{equation}

\subsection{Modifying LS:  Adding or Deleting a Data/Row}\label{section:cholesky-rank-one-update}

Updating linear systems after low-rank modifications of the system matrix is a  common practice in machine learning, statistics, and many other fields. 
However, it is widely recognized that such updates can introduce significant instabilities due to round-off errors \citep{seeger2004low}. 
When the system matrix is positive definite, employing a representation based on Cholesky decomposition is usually preferable as it provides improved numerical stability \citep{gill1974methods, bojanczyk1987note,  chang1997pertubation, davis1999modifying, seeger2004low, chen2008algorithm, davis2008user, higham2009cholesky}. 

On the other hand, many applications require the solution of a least squares problem after the data have been modified by adding (updating) or deleting (downdating) observations. Examples arise in regression problems, optimization, signal processing, and prediction in control
theory \citep{bjorck2024numerical}.
In this section, we will present a proof for the rank-one update/downdate using Cholesky decomposition.
\subsection*{Rank-One Update}\index{Rank-one update}

Note that we follow the dimension notation of the underlying matrix for the Cholesky decomposition such that $\bA\in\real^{n\times n}$. However, the LS notation uses $\bA=\bX^\top\bX \in\real^{p\times p}$. 

A rank-one update $\bA^\prime$ of a matrix $\bA$ by a new data vector $\bv$ is of the form:
\begin{equation*}
\begin{aligned}
\bA^\prime &= \bA + \bv \bv^\top,   \quad \text{with} \quad \bA\triangleq \bX^\top\bX \ \text{ and }\ \bA^\prime\triangleq \begin{bmatrixfoot}
\bX^\top &
\bv^\top 
\end{bmatrixfoot}^\top \begin{bmatrixfoot}
\bX \\
\bv^\top
\end{bmatrixfoot},\\
\downarrow &\gap  \downarrow\\
\bR^{\prime\top}\bR^\prime &= \bR^\top\bR + \bv \bv^\top.
\end{aligned}
\end{equation*}
If the Cholesky factor $\bR$ of $\bA \in \real^{n\times n}$ has already been computed, we can efficiently obtain the Cholesky factor $\bR^\prime$ of $\bA^\prime$.
Note that $\bA^\prime$ differs from $\bA$ only by the symmetric rank-one matrix. 
Therefore, we can compute $\bR^\prime$ from $\bR$ using the rank-one Cholesky update, which takes $\mathcalO(n^2)$ operations, each saving from $\mathcalO(n^3)$ complexity if we were to recompute the Cholesky decomposition of $\bA^\prime$ from scratch,
given that we  know $\bR$, the Cholesky decomposition of $\bA$ up front. 
That is, we want to compute the Cholesky decomposition of $\bA^\prime$ via that of $\bA$. 
To see this,
suppose there exists a set of orthogonal matrices $\bQ_n \bQ_{n-1}\ldots \bQ_1$ such that  
$$
\bQ_n \bQ_{n-1}\ldots \bQ_1 
\begin{bmatrix}
\bv^\top \\
\bR
\end{bmatrix}
=
\begin{bmatrix}
\bzero \\
\bR^\prime
\end{bmatrix}.
$$
Then we find out the expression for the Cholesky factor of $\bA^\prime$ by $\bR^\prime$. 
Specifically, multiplying the left-hand side of the above equation by its transpose yields
$$
\begin{bmatrix}
\bv & \bR^\top
\end{bmatrix}
\bQ_1^\top \ldots \bQ_{n-1}^\top\bQ_n^\top
\bQ_n \bQ_{n-1}\ldots \bQ_1 
\begin{bmatrix}
\bv^\top \\
\bR
\end{bmatrix}
= \bR^\top\bR + \bv \bv^\top.
$$ 
Similarly, multiplying the right-hand side by its transpose gives
$$
\begin{bmatrix}
\bzero & \bR^{\prime\top}
\end{bmatrix}
\begin{bmatrix}
\bzero \\
\bR^\prime
\end{bmatrix}=\bR^{\prime\top}\bR^\prime,
$$
which matches the left-hand side equation. \textit{Givens rotations} are examples of such orthogonal matrices that can transfer $\bR$ and $\bv$ into $\bR^\prime$. 

\begin{definition}[$n$-th Order Givens Rotation]\label{definition:givens-rotation-in-qr}
A \textit{Givens rotation} is represented by a matrix of the following form
$$
\bG_{kl}=
\begin{bmatrix}
1 &          &   &  &   &   & &  & &\\
& \ddots  &  &  &  & && & &\\
&      & 1 &  & & &  && &\\
&      &  & c &  &  &  & s & &\\
&& &   & 1 & & && &\\
&& &   &   &\ddots &  && &\\
&& &  &   &  & 1&& &\\
&& & -s &  &  & &c& &\\
&& & &  &  & & &1 & \\
&& & &  &  & & & &\ddots
\end{bmatrix}_{n\times n},
$$
where the $(k,k), (k,l), (l,k), (l,l)$ entries are $c, s, -s, c$, respectively, and $s = \sin \theta$ and $c=\cos \theta$ for some $\theta$.

Let $\bdelta_k \in \real^n$ be the zero vector except that the entry $k$ is 1 (the $k$-th unit basis vector). Then, mathematically, the Givens rotation defined above can be denoted by 
$$
\bG_{kl}\triangleq\bG_{kl}(\theta)= \bI + (c-1)(\bdelta_k\bdelta_k^\top + \bdelta_l\bdelta_l^\top) + s(\bdelta_k\bdelta_l^\top -\bdelta_l\bdelta_k^\top ),
$$
where the subscripts $k$ and $l$ indicate that the rotation occurs \textbf{in plane $k$ and $l$}.

Specifically, one can also define the $n$-th order Givens rotation, where $(k,k),$ $(k,l),$ $(l,k),$ $(l,l)$ entries are $c, \textcolor{mylightbluetext}{-s, s}, c$ respectively (note the difference in the sign of $s$). The underlying principles remain the same.
\end{definition}

\begin{exercise}
Show that $\bG_{kl}(-\theta)^{-1}=\bG_{kl}(\theta)$. \textit{Hint: Use the orthogonality of $\bG_{kl}(\theta)$.}
\end{exercise}

It can be easily verified that an $n$-th order Givens rotation is an orthogonal matrix, and its determinant is 1. For any vector $\bx =[x_1, x_2, \ldots, x_n]^\top \in \real^n$, the result of applying a Givens rotation $\bG_{kl}$ to $\bx$ is $\by = \bG_{kl}\bx$, where
$$ 
\left\{
\begin{aligned}
&y_k = c \cdot x_k + s\cdot x_l,   \\
&y_l = -s\cdot x_k +c\cdot x_l,  \\
&y_j = x_j . &  (j\neq k,l) 
\end{aligned}
\right.
$$
That is, a Givens rotation applied to $\bx$ rotates two components of $\bx$ by some angle $\theta$, while keeping all other components unchanged.

Now, let's consider a Givens rotation of order $(n+1)$, where the rotation is indexed from $0$ to $n$. This rotation can be expressed as 
$$
\bG_k \triangleq \bI + (c_k-1)(\bdelta_0\bdelta_0^\top + \bdelta_k\bdelta_k^\top) + s_k(\bdelta_0\bdelta_k^\top -\bdelta_k\bdelta_0^\top ),
$$
where $c_k = \cos \theta_k, s_k=\sin\theta_k$ for some $\theta_k$, $\bG_k \in \real^{(n+1)\times (n+1)}$, and $\bdelta_k\in \real^{n+1}$ is a zero vector except that the $(k+1)$-th entry is 1.

\begin{mdframed}[hidealllines=\mdframehideline,backgroundcolor=\mdframecolor]
Taking out the $k$-th column of the following equation 
$$
\begin{bmatrix}
\bv^\top \\
\bR
\end{bmatrix}
\longrightarrow 
\begin{bmatrix}
\bzero \\
\bR^\prime
\end{bmatrix},
$$
where we let the $k$-th element of $\bv$ be $v_k$, and the $k$-th diagonal of $\bR$ be $r_{kk}$.
We realize that $\sqrt{v_k^2 + r_{kk}^2} \neq 0$,
and let $c_k \triangleq \frac{r_{kk}}{\sqrt{v_k^2 + r_{kk}^2}}$, $s_k\triangleq-\frac{v_k}{\sqrt{v_k^2 + r_{kk}^2}}$. Then we have
$$ 
\left\{
\begin{aligned}
&v_k \rightarrow c_kv_k+s_kr_{kk}=0;   \\
&r_{kk}\rightarrow -s_k v_k +c_kr_{kk}= \sqrt{v_k^2 + r_{kk}^2} = r^\prime_{kk} .  \\
%&\bR^\prime_{kk}= \bR_{jl} . &  (j\neq k,l) 
\end{aligned}
\right.
$$
That is, $\bG_k$ will introduce a zero value to the $k$-th element of $\bv$ and a nonzero value to $r_{kk}$.
\end{mdframed}
This finding above is crucial for the rank-one update. And we obtain 
$$
\bG_n \bG_{n-1}\ldots \bG_1 
\begin{bmatrix}
\bv^\top \\
\bR
\end{bmatrix}
=
\begin{bmatrix}
\bzero \\
\bR^\prime
\end{bmatrix}.
$$
For each Givens rotation, it takes $6n$ flops. And there are $n$ such rotations, which requires $6n^2$ flops if keeping only the leading term. The complexity to calculate the Cholesky factor of $\bA^\prime$ is thus reduced from $\frac{1}{3} n^3$ to $6n^2$ flops using the rank-one update, provided that we already know the Cholesky factor of $\bA$. 

The above algorithm is also essential for reducing the complexity in the posterior calculation of Bayesian inference for Gaussian mixture model \citep{lu2021bayes}. At each stage, $k$ new samples are added or removed from an existing cluster, which corresponds to applying $k$ rank-one updates.

\subsection*{Rank-One Downdate}\index{Rank-one downdate}
Let us consider the scenario where we have computed the Cholesky factor of $\bA$, and $\bA^\prime$ is the rank-one downdate of $\bA$ given by the following expression:
\begin{equation*}
\begin{aligned}
\bA^\prime &= \bA - \bv \bv^\top;\\
\downarrow &\gap  \downarrow\\
\bR^{\prime\top}\bR^\prime &= \bR^\top\bR - \bv \bv^\top. 
\end{aligned}
\end{equation*}
The algorithm is similar by proceeding as follows:
\begin{equation}\label{equation:rank-one-downdate}
\bG_1 \bG_{2}\ldots \bG_n
\begin{bmatrix}
\bzero \\
\bR
\end{bmatrix}
=
\begin{bmatrix}
\bv^\top \\
\bR^\prime
\end{bmatrix}.
\end{equation}
Again, the set of Givens rotations $
\bG_k = \bI + (c_k-1)(\bdelta_0\bdelta_0^\top + \bdelta_k\bdelta_k^\top) + s_k(\bdelta_0\bdelta_k^\top -\bdelta_k\bdelta_0^\top )
$
for $k\in\{1,2,\ldots,n\}$
can be constructed as follows:
\begin{mdframed}[hidealllines=\mdframehideline,backgroundcolor=\mdframecolor]
Taking out the $k$-th column of the following equation 
$$
\begin{bmatrix}
\bzero \\
\bR
\end{bmatrix}
\longrightarrow
\begin{bmatrix}
\bv^\top \\
\bR^\prime
\end{bmatrix}.
$$
We realize that $r_{kk} \neq 0$,
and let $c_k\triangleq\frac{\sqrt{r_{kk}^2 - v_k^2}}{r_{kk}}$, $s_k \triangleq \frac{v_k}{r_{kk}}$. Then, we have
$$ 
\left\{
\begin{aligned}
& 0 \rightarrow s_kr_{kk}=v_k;   \\
&r_{kk}\rightarrow c_k r_{kk}= \sqrt{r_{kk}^2-v_k^2  }=r^\prime_{kk} .  \\
%&\bR^\prime_{kk}= \bR_{jl} . &  (j\neq k,l) 
\end{aligned}
\right.
$$
This requires $r^2_{kk} > v_k^2$ to make $\bA^\prime$ to be positive definite. Otherwise, $c_k$ as defined above will not exist.
\end{mdframed}
Again, one can verify that, multiplying the left-hand side of Equation~\eqref{equation:rank-one-downdate} by its transpose, we have 
$$
\begin{bmatrix}
\bzero & \bR^\top
\end{bmatrix}
\bG_n^\top\ldots \bG_{2}^\top\bG_1^\top
\bG_1 \bG_{2}\ldots \bG_n
\begin{bmatrix}
\bzero \\
\bR
\end{bmatrix} =\bR^\top\bR.
$$
And multiplying the r.h.s. by its transpose, we have 
$$
\begin{bmatrix}
\bv & \bR^{\prime\top} 
\end{bmatrix}
\begin{bmatrix}
\bv^\top \\
\bR^\prime
\end{bmatrix}=\bv\bv^\top + \bR^{\prime\top}\bR^\prime.
$$
This results in $\bR^{\prime\top}\bR^\prime = \bR^\top\bR - \bv \bv^\top$.

\index{Rank-revealing}\index{Semidefinite rank-revealing}
\subsection{Rank-Deficient Case}

If the columns of $ \bX \in \real^{n\times p} $ are linearly dependent,  then $\rank(\bX ) = r < p$, and the matrix appearing in the normal equation $ \bX^\top \bX  $ is positive semidefinite. 
In this case, the Cholesky factor $ \bR $ must have $ n - r $ zero diagonal elements.
By employing symmetric pivoting during the factorization process, these zero entries can be arranged to appear at the end of the diagonal.
\begin{theoremHigh}[Semidefinite Cholesky decomposition, a.k.a., semidefinite rank-revealing decomposition]\label{theorem:semi_cholesky}
Let $ \bA \in \real^{n\times n} $ be a symmetric positive semidefinite matrix of rank $ r < n $. Then it can be  factored as 
$$
\bP^\top \bA  \bP = \bR^\top \bR, 
\quad \text{with}\quad  
\bR = 
\begin{bmatrix}
\bR_{11} & \bR_{12} \\
\bzero & \bzero
\end{bmatrix},
$$
where $\bP$ is a permutation matrix, and $ \bR_{11} \in \real^{r \times r} $ is upper triangular with positive diagonal elements.
Although such decompositions for $\bA$ is not unique, the decomposition for $\bP^\top\bA\bP$ is unique.
\end{theoremHigh}
\begin{proof}[of Theorem~\ref{theorem:semi_cholesky}]
The proof is constructive and follows a similar approach to the second perspective used for computing the Cholesky decomposition, as described in Section~\ref{section:ls_cholesky}.
The algorithm begins with  $ \bA^{(1)} = \bA  $ and generates a sequence of matrices defined as
$$
\bA^{(k)} = [a_{ij}^{(k)}] = 
\begin{bmatrixfoot}
\bI_{k-1} & \bzero \\
\bzero & \bB^{(k)}
\end{bmatrixfoot}, \quad k = 1, 2, \ldots,\quad\text{with } \bB^{(k)}\in\real^{(n-k+1)\times (n-k+1)}.
$$
At the beginning  of step $ k $ we select the largest diagonal element of $ \bA^{(k)} $,
$$
s_q^{(k)} = \max_{k \leq i \leq n} a_{ii}^{(k)},
$$
and interchange rows and columns $q$ and $k$ to bring this into pivot position; that is, $s_p$ appears in the ($k-1,k-1$) position of $\bP^{(k)\top}\bA^{(k)}\bP^{(k)}$, where the permutation matrix $\bP^{(k)}$ has the form 
$$
\bP^{(k)}\triangleq \begin{bmatrixfoot}
\bI_{k-1} & \bzero \\
\bzero & \widetildebP^{(k)}
\end{bmatrixfoot},
$$ 
and $\widetildebP^{(k)}\in\real^{(n-k+1)\times (n-k+1)}$ is a smaller permutation matrix. 
This pivot must be positive for $k < r$, because otherwise $\bB^{(k)} = \bzero$, which implies that $\rank(\bA) < r$. Next, the elements in the permuted $\bA^{(k)}$ are transformed according to the Cholesky Algorithm~\ref{alg:compute-choklesky11}:
$$
\begin{aligned}
r_{kk} &= \sqrt{a_{kk}^{(k)}}, \quad r_{kj} = a_{kj}^{(k)} / r_{kk}, \quad j = k + 1:n,\\
a_{ij}^{(k+1)} &= a_{ij}^{(k)} - r_{ki} r_{kj}, \quad i,j = k + 1:n.
\end{aligned}
$$
This process is equivalent to subtracting a symmetric rank-one matrix $\br_j \br_j^\top$ from $\bA^{(k)}$, where $\br_j = \be_j^\top \bR$ is the $j$-th row of $\bR$. The algorithm stops when $k = r + 1$. Then all the remaining diagonal elements are zero, which implies that $\bA^{(r+1)} = \begin{bmatrixfoot}
\bI_r & \bzero \\
\bzero & \bzero 
\end{bmatrixfoot}$.

\paragraph{Construction algorithm.} Below contains more constructive analysis.
Following the  second perspective for computing the Cholesky decomposition in Section~\ref{section:ls_cholesky}, 
we can construct 
$$
\bP^{(k)\top}\bA^{(k)} \bP^{(k)}\triangleq
\begin{bmatrixfoot} 
\bI_{k-1} & 0 & \bzero \\ 
0 & a_{kk}^{(k)} & \bb_k^\top \\ 
\bzero & \bb_k & \bB^{(k)} 
\end{bmatrixfoot} 
\qquad\text{and}\qquad
\bL^{(k)} = 
\begin{bmatrixfoot} 
\bI_{k-1} & 0 & \bzero \\ 
0 & \sqrt{a_{kk}^{(k)}} & \bzero \\ 
\bzero & \frac{1}{\sqrt{a_{kk}^{(k)}}}\bb_k & \bI_{n-k} 
\end{bmatrixfoot},
$$
satisfying  $\bP^{(k)\top}\bA^{(k)}\bP^{(k)} = \bL^{(k)}\bA^{(k+1)}(\bL^{(k)})^\top$:
$$
\begin{aligned} \bA^{(k+1)} 
&= \begin{bmatrixfoot} 
\bI_{k-1} & 0 & \bzero \\ 
0 & 1 & \bzero \\ 
\bzero & \bzero & \bB^{(k)}-\frac{1}{a_{kk}^{(k)}}\bb_k\bb_k^\top 
\end{bmatrixfoot} 
\triangleq 
\begin{bmatrixfoot} 
\bI_k & 0 & \bzero \\ 
0 & a_{k+1, k+1}^{(k+1)} & \bb_{k+1}^\top \\ 
\bzero & \bb_{k+1} & \bB^{(k+1)} 
\end{bmatrixfoot}.
\end{aligned}
$$

However, we notice that these permutation matrices $\bP^{(1)}, \bP^{(2)}, \ldots, \bP^{(r)}$ are used to permute two columns; therefore, they are symmetric satisfying $\bP^{(k)}\cdot \bP^{(k)} = \bI$ for all $k$.
Let $\bP\triangleq \bP^{(1)}\bP^{(2)}\ldots\bP^{(r)}$. 
Since $(\bA^{(r+1)})^2=\bA^{(r+1)}$, $\bA^{(1)}=\bA$ can be expressed as 
\begin{align}
&\bP^\top\bA^{(1)}\bP \triangleq \bL\bL^\top;\\
&\bL\triangleq \left\{\bP^{(r)}\bP^{(r-1)}\ldots \bP^{(2)} \bP^{(1)}\right\}
\left\{\bP^{(1)}\bL^{(1)} \right\} \left\{\bP^{(2)}\bL^{(2)} \right\}
\ldots \left\{\bP^{(r)}\bL^{(r)} \right\} \bA^{(r+1)}. \label{equation:rrcho_low1}
\end{align}
To complete the proof, it suffice to show that $\bL$ is lower triangular with the rank-revealing property.
On the other hand, each lower triangular $\bL^{(k)}$ can be written as 
$$
\bL^{(k)} = 
\bI - \bl_k \be_k^\top
\quad \text{with}\quad \bl_k = [\bzero_{k-1}, l_{k}, l_{k+1}, \ldots, l_n]^\top,
$$
where $\be_k$ is the $k$-th standard unit basis, and $\bl_k$ is a vector containing $k-1$ zeros. Note that $1-l_k \equiv \sqrt{a_{kk}^{(k)}}$ in this notation.
For $k\in\{1,2,\ldots,r-1\}$, define 
$$
\begin{aligned}
\bM_k&\triangleq \bP^{(r)}\bP^{(r-1)}\ldots \bP^{(k+1)}\bL^{(k)}\bP^{(k+1)}\ldots \bP^{(r-1)}\bP^{(r)}\\
&=\bP^{(r)}\bP^{(r-1)}\ldots \bP^{(k+1)} (\bI - \bl_k\be_k^\top)\bP^{(k+1)}\ldots \bP^{(r-1)}\bP^{(r)}\\
&=\bI - (\bP^{(r)}\bP^{(r-1)}\ldots \bP^{(k+1)} \bl_k) (\be_k^\top\bP^{(k+1)}\ldots \bP^{(r-1)}\bP^{(r)})\\
&=\bI - (\bP^{(r)}\bP^{(r-1)}\ldots \bP^{(k+1)} \bl_k) \be_k^\top,
\end{aligned}
$$
where the last equality follows since  $\be_k^\top\bP^{(k+1)}\ldots \bP^{(r-1)}\bP^{(r)}=\be_k^\top$.
This implies $\bM_k$ is lower triangular with its $k$-th column representing a permuted version of $\bL^{(k)}$.
Therefore, it holds that 
$$
\begin{aligned}
&\bM_1 \bM_2\ldots \bM_{r-1} = 
\left\{\bP^{(r)}\bP^{(r-1)}\ldots \bP^{(2)}\right\} \left\{\bL^{(1)} \bP^{(2)}\right\} 
\left\{\bL^{(2)}\bP^{(3)}\right\} \ldots \left\{\bL^{(r-1)}\bP^{(r)}\right\};\\
&\bL \equiv \bM_1 \bM_2\ldots \bM_{r-1} \bL^{(r)}\bA^{(r+1)}.
\end{aligned}
$$
From the above analysis, $\bM_1 \bM_2\ldots \bM_{r-1}$ is lower triangular, and $\bL^{(r)}\bA^{(r+1)}$ has the form
$$
\bL^{(r)}\bA^{(r+1)}
=
\begin{bmatrixfoot}
\bM_{11} &\bzero \\
\bM_{21} &\bzero
\end{bmatrixfoot},
\quad \text{with lower triangular $\bM_{11}$}. 
$$
Therefore, $\bL = \bM_1 \bM_2\ldots \bM_{r-1} \bL^{(r)}\bA^{(r+1)}$ has the desired form
$$
\bL=
\begin{bmatrixfoot}
\bL_{11} &\bzero \\
\bL_{21} &\bzero
\end{bmatrixfoot},
\quad 
\text{with lower triangular $\bL_{11}$}. 
$$
This completes the proof.
\end{proof}

Since all the reduced matrices $\bA^{(k)}$ are symmetric positive semidefinite, their maximum elements lie on the diagonal (Corollary~\ref{corollary:sylt3}). 
Hence, the pivot selection the proof  described above is equivalent to \textit{complete pivoting}. The algorithm produces a matrix $\bR$ whose diagonal elements in $\bR$ form a nonincreasing sequence 
\begin{equation}
r_{11} \geq r_{22} \geq \ldots \geq r_{nn}.
\end{equation}
In fact, the following stronger inequalities also hold:
\begin{equation}
r_{kk}^2 \geq \sum_{i=k}^j r_{ii}^2, \qquad j = k + 1, \ldots, n, \ k = 1,2,\ldots,r;
\end{equation}
\textcolor{black}{see Section~\ref{section:piv_qr}}.

The proof given above is constructive and can be directly used to compute the semidefinite Cholesky decomposition; see Algorithm~\ref{alg:semid_cholesky}.
However, this approach may appear somewhat complicated. A more concise proof is presented below using the spectral decomposition (Theorem~\ref{theorem:spectral_theorem}) and the column-pivoted QR decomposition (which will be introduced in Theorem~\ref{theorem:rank-revealing-qr-general}).

\begin{proof}[of Theorem~\ref{theorem:semi_cholesky}: an alternative proof]
The ``nonsingular" factor of the PSD matrix $\bA$ is given by $\bA = \bZ^\top\bZ$, where $\bZ=\bLambda^{1/2}\bQ^\top$, and $\bA=\bQ\bLambda\bQ^\top$ is the spectral decomposition of $\bA$.
The rank of matrix $\bA$ is the number of nonzero eigenvalues (here, the number of positive eigenvalues since $\bA$ is PSD). Therefore, only $r$ components in $\bLambda^{1/2}$ are nonzero, and $\bZ=\bLambda^{1/2}\bQ^\top$ contains only $r$ independent columns, i.e., $\bZ$ is of rank $r$. By utilizing the column-pivoted QR decomposition, we have
$$
\bZ\bP = \bQ
\begin{bmatrix}
\bR_{11} & \bR_{12} \\
\bzero   & \bzero 
\end{bmatrix},
$$
where $\bP$ is a permutation matrix, $\bR_{11}\in \real^{r\times r}$ is upper triangular with positive diagonals, and $\bR_{12}\in \real^{r\times (n-r)}$. Therefore,
$$
\bP^\top\bA\bP  = 
\bP^\top\bZ^\top\bZ\bP = 
\begin{bmatrix}
\bR_{11}^\top & \bzero \\
\bR_{12}^\top & \bzero 
\end{bmatrix}
\begin{bmatrix}
\bR_{11} & \bR_{12} \\
\bzero   & \bzero 
\end{bmatrix},
\gap \text{with}\gap 
\bR \triangleq \begin{bmatrix}
\bR_{11} & \bR_{12} \\
\bzero   & \bzero 
\end{bmatrix}.
$$
Thus, we find the rank-revealing decomposition for the semidefinite matrix: $\bP^\top\bA\bP = \bR^\top\bR$.
\end{proof}

\begin{algorithm}[h] 
\caption{Semidefinite Cholesky Decomposition}  
\label{alg:semid_cholesky} 
\begin{algorithmic}[1] 
\Require 
Positive semidefinite matrix $\bA$ with size $n\times n$; 
\State Set $ p_i = i $, $ i = 1,2,\ldots, n $;
\State $\bR\leftarrow \bA$
\For{$k=1$ to $n$}  %\Comment{Compute the $j$-th column of $\bL$}
\State Find $ s $ such that $ r_{ss} = \max_{k \leq i \leq n} r_{ii} $;
\State Swap rows and columns $ k $ and $ s $ of $ \bA $ and swap $ p_k $ and $ p_s $;
\State $ r_{kk} = \sqrt{r_{kk}} $;
\For{$j=k + 1$ to $n$}
\State $ r_{kj} = r_{kj} / r_{kk} $;
\EndFor
\For{$j=k + 1$ to $n$}
\For{$i=k + 1$ to $j$}
\State $ r_{ij} = r_{ij} - r_{ki} r_{kj} $;
\EndFor
\EndFor
%\State $l_{jj}\leftarrow1$;
%\State $c_{jj}\leftarrow r_{jj}-\sum_{k=1}^{j-1}d_{kk} l_{jk}^2$;
%\State $d_{jj}\leftarrow c_{jj}$
%\For{$i=j+1$ to $n$} 
%\State $c_{ij}\leftarrow r_{ij}-\sum_{k=1}^{j-1}d_{kk} l_{ik}l_{jk}$, since $i>j$;
%\State $l_{ij}\leftarrow \frac{c_{ij}}{d_{jj}}$;
%\EndFor
\EndFor
\State  Set $ \bP $ to the matrix whose $ j $-th column is the $ p_j $-th column of $ \bI $;
%\State Output $\bA=\bL\bD\bL^\top$, where $\bD=\diag(d_{11}, d_{22},\ldots,d_{nn})$.
\State Output $\bP, \bR$.
\end{algorithmic} 
\end{algorithm}

Rounding errors can cause negative elements to appear on the diagonal in the Cholesky algorithm, even if $\bZ$ is positive semidefinite. Similarly, the computed reduced matrix will in general be nonzero after $r$ steps even when $\rank(\bZ) = r$. 
This situation raises questions about the appropriate time to terminate the Cholesky factorization of a semidefinite matrix. One approach is to stop the process when
$$
\max_{k \leq i \leq n} a_{ii}^{(k)} \leq 0,
$$
and set $\rank(\bZ) = k - 1$. 
However, this approach might lead to unnecessary computations in eliminating negligible elements. Considering the computational cost, we recommend using the following stopping criterion:
$$
\max_{k \leq i \leq n} a_{ii}^{(k)} \leq c_n \epsilon r_{11}^2,
$$
where $\epsilon$ denotes the unit roundoff error, $c_n$ is a modest constant \citep{higham1989accurate, higham2002accuracy}.

\index{Normal equation}
\paragrapharrow{Rank-deficient normal equation.}
In the context of least squares problem with the normal equation $\bX^\top\bX\bbeta = \bX^\top\by$, where $\bP^\top\bX^\top\bX\bP = \bR^\top\bR$ is the semidefinite Cholesky decomposition.
In the rank-deficient case, the permuted normal equation become
$$
\bR^\top \bR \widetildebbeta = \widetildebalpha, \qquad \bbeta = \bP \widetildebbeta, \qquad \widetildebalpha = \bP^\top (\bX^\top \by).
$$
With $\widetildebu \triangleq \bR \widetildebbeta$, we obtain
$$
\bR^\top \widetildebu 
\triangleq 
\begin{bmatrix}
\bR_{11}^\top & \bzero \\
\bR_{12}^\top & \bzero 
\end{bmatrix} 
\begin{bmatrix}
\widetildebu_1 \\
\widetildebu_2
\end{bmatrix}
 = \begin{bmatrix}
\widetildebalpha_1 \\
\widetildebalpha_2
\end{bmatrix}
\triangleq \widetildebalpha
,
$$
where $\bR_{11} \in \real^{r \times r}$ is nonsingular. The triangular system $\bR_{11}^\top \widetildebu_1 = \widetildebalpha_1$ determines $\widetildebu_1 \in \real^r$. From
\begin{equation}\label{equation:basic_def_cho1}
\bR_{11} \widetildebbeta_1 = \widetildebu_1 - \bR_{12} \widetildebbeta_2,
\end{equation}
where $\widetildebbeta = [\widetildebbeta_1^\top, \widetildebbeta_2^\top]^\top$, we can determine $\widetildebbeta_1$ for an arbitrarily chosen $\widetildebbeta_2$. 
This reflects the fact that a consistent singular system has infinitely many solutions. 
Finally, the permutations are undone to determine $\bbeta = \bP \widetildebbeta$.

Setting $\widetildebbeta_2 = \bzero$ we get a basic solution $\widehatbbeta_{r} = [\widehatbbeta_{r1}^\top, \bzero]^\top$ with only $r = \rank(\bX)$ nonzero components in $\bbeta$, corresponding to the first $r$ columns in $\bX \bP$. 
This is particularly useful when seeking a good least squares fit of $\by$ using as few variables as possible (variable selection; see Section~\ref{section:variable-selection}).

\index{Minimum-norm solution}
\paragrapharrow{Minimum-norm solution.}

The minimum-norm solution $\widehatbbeta=\bX^+\by$ is the one that minimizes $\normtwo{\bbeta} = \normtwobig{\widetildebbeta}$. 
After obtaining the basic solution $\widehatbbeta_{r}$, $\widetildebu_1$ can be determined from \eqref{equation:basic_def_cho1} as $\widetildebu_1 = \bR_{11} \widehatbbeta_{r1} $.
This again shows that any solution  $\widetildebbeta = [\widetildebbeta_1^\top, \widetildebbeta_2^\top]^\top$ satisfies
\begin{equation}
\widetildebu_1 = \bR_{11} \widehatbbeta_{r1} 
=
\bR_{11} \widetildebbeta_1 + \bR_{12} \widetildebbeta_2
\qquad \implies\qquad 
\widetildebbeta_1 =  \widehatbbeta_{r1}  - \bR_{11}^{-1}\bR_{12} \widetildebbeta_2.
\end{equation}
Thus,   the minimum-norm of $\widetildebbeta$ can be obtained from the full-rank least squares problem
\begin{equation}\label{equation:qrfulo_minnom}
\min_{\widetildebbeta_2\in\real^{p-r}} 
\normtwo{\begin{bmatrix} 
\bF  \\ 
-\bI_{p-r} 
\end{bmatrix} 
\widetildebbeta_2 - 
\begin{bmatrix} 
\widehatbbeta_{r1} \\ \bzero
\end{bmatrix}
}, 
\qquad \bF \triangleq \bR_{11}^{-1} \bR_{12}\in\real^{r\times (p-r)}.
\end{equation}
The basic solution $\widehatbbeta_{r1}$ can be computed in about $r^2(p-r)$ flops. Note that $\bF$ can overwrite $\bR_{12}$. Then $\bbeta_2$ can be computed from the normal equation,
$$
(\bF^\top \bF + \bI_{p-r}) \widetildebbeta_2 = \bF^\top \widehatbbeta_{r1},
$$
using a Cholesky decomposition of $(\bF^\top \bF + \bI_{p-r})$. When $\widetildebbeta_2$ has been determined, we have $\widetildebbeta_1 = \widehatbbeta_{r1} - \bF \widetildebbeta_2$. This method requires about $r(p-r)^2 + \frac{1}{2}(p-r)^3$ flops \citep{deuflhard1980rank}.
The final solution once again involves an undo of the permutation $\bP \widetildebbeta$.

\index{Full column rank}
\index{QR decomposition}
\section{LS via  QR Decomposition}\label{section:ls_qr_gen}

\subsection{Full Rank Case}\label{section:application-ls-qr}

We begin by presenting the least squares solution for the full-rank case using QR decomposition. Other related problems---such as the minimum-norm problem---will be discussed later.
In the least squares solution (Theorem~\ref{theorem:ols}), 
computing the inverse of $\bX^\top\bX$ can be numerically unstable or computationally intensive. To avoid this, we can instead use the QR decomposition to compute the least squares solution more efficiently and accurately, as shown in the following theorem.

\begin{theoremHigh}[LS via QR for full column rank matrix]\label{theorem:qr-for-ls}
Let $\bX\in\real^{n\times p}$ with $n\geq p$ and full column rank, and suppose  $\bX=\bQ\bR$ is its full QR decomposition, where $\bQ\triangleq[\bQ_1,\bQ_2]\in\real^{n\times n}$ ($\bQ_1\in\real^{n\times p}$ and $\bQ_2\in\real^{n\times (n-p)}$) is an  orthogonal matrix, $\bR= \begin{bmatrixfoot}
	\bR_1 \\
	\bzero
\end{bmatrixfoot}\in \real^{n\times p}$ is an upper triangular matrix appended by additional $n-p$ zero rows, and $\bR_1 \in \real^{p\times p}$ is the square upper triangular matrix within $\bR$. 
Then for any response vector $\by\in \real^n$, the LS solution to $\bX\bbeta=\by$ is given by 
$$
\widehatbbeta = \bR_1^{-1}\bc_1, 
\quad \text{with }
\begin{bmatrix}
\bc_1 \\
\bc_2 
\end{bmatrix}
\triangleq \bQ^\top\by,
$$
where $\bc_1$ contains the first $p$ components of $\bQ^\top\by$.
The error component is 
$$
\be = \by-\bX\widehatbbeta = \bQ
\begin{bmatrix}
\bzero \\
\bc_2
\end{bmatrix}.
$$
\end{theoremHigh}

\begin{proof}[of Theorem~\ref{theorem:qr-for-ls}]
Since $\bX=\bQ\bR$ is the full QR decomposition of $\bX$ and $n\geq p$, the last $n-p$ rows of $\bR$ are zero as shown in Figure~\ref{fig:qr-comparison}. Then, $\bR_1 \in \real^{p\times p}$ is the square upper triangular in $\bR$ and 
$
\bQ^\top \bX = \bR = 
\footnotesize
\begin{bmatrix}
\bR_1 \\
\bzero
\end{bmatrix}.
$
Write out the loss function,
$$
\begin{aligned}
\normtwo{\by-\bX\bbeta}^2 
%&= (\by-\bX\bbeta)^\top(\by-\bX\bbeta)
%\stackrel{\dag}{=}(\by-\bX\bbeta)^\top\bQ\bQ^\top (\by-\bX\bbeta) \\
&\stackrel{\ddag}{=}\normtwo{\bQ^\top \bX \bbeta-\bQ^\top\by}^2 
=\normtwo{
\begin{bmatrixfoot}
\bR_1 \\
\bzero
\end{bmatrixfoot} 
\bbeta-\bQ^\top\by}^2
=\normtwo{\bR_1\bbeta - \bc_1}^2+\normtwo{\bc_2}^2,
\end{aligned}
$$ 
where $\bc_1$ is the first $p$ components of $\bQ^\top\by$,  $\bc_2$ is the last $n-p$ components of $\bQ^\top\by$,  and the equality ($\ddag$) follows from the invariance under orthogonal transformations.
Then the OLS solution can be calculated by performing backward substitution on the upper triangular system $\bR_1\bbeta = \bc_1$, i.e., $\widehatbbeta = \bR_1^{-1}\bc_1$.
\end{proof}
\index{Floating-point operation}

Note that when $\rank(\bX) = p$, the pseudo-inverses of $\bX$ and $\bX^\top$ can be expressed in terms of the QR factorization as
$$
\bX^+ = \bR_1^{-1} \bQ_1^\top 
\qquad \text{and}\qquad 
(\bX^\top)^+ = \bQ_1 \bR_1^{-\top}.
$$
Since $\widehatbbeta=\bX^+\by$, we again obtain by Theorem~\ref{theorem:ols} that 
$$
\widehatbbeta=\bX^+\by=\bR_1^{-1} \bQ_1^\top\by, 
\quad \text{with }\bc_1\equiv  \bQ_1^\top\by.
$$

The inverse of an upper triangular matrix $\bR_1\in \real^{p\times p}$ requires $\frac{1}{3}p^3$ flops. However, the inverse of a basic $p\times p$ nonsingular matrix (in our case, the inverse of $\bX^\top\bX$) requires $2p^3$ flops \citep{lu2021numerical, lu2022matrix}. 
Therefore, using QR decomposition for OLS instead of directly inverting matrices significantly reduces computational complexity.

%To verify Theorem~\ref{theorem:qr-for-ls}, given the full QR decomposition of $\bX = \bQ\bR$, where $\bQ\in \real^{n\times n}$ and $\bR\in \real^{n\times p}$. Together with the OLS solution from calculus, we obtain
%\begin{equation}\label{equation:qr-for-ls-1}
%\begin{aligned}
%\widehatbbeta &= (\bX^\top\bX)^{-1}\bX^\top\by
%= (\bR^\top\bQ^\top\bQ\bR)^{-1} \bR^\top\bQ^\top \by
%%= (\bR^\top\bR)^{-1} \bR^\top\bQ^\top \by \\
%= (\bR_1^\top\bR_1)^{-1} \bR^\top\bQ^\top \by \\
%&=\bR_1^{-1} \bR_1^{-\top} \bR^\top\bQ^\top \by
%= \bR_1^{-1} \bR_1^{-\top} \bR_1^\top\bQ_1^\top \by 
%=\bR_1^{-1} \bQ_1^\top \by,
%\end{aligned}
%\end{equation}
%where $\bR = \begin{bmatrixfoot}
%\bR_1 \\
%\bzero
%\end{bmatrixfoot}$ 
%and $\bR_1\in \real^{p\times p}$  is an upper triangular matrix, and $\bQ_1 =\bQ_{1:n,1:p}\in \real^{n\times p}$ contains the first $p$ columns of $\bQ$ (i.e., $\bQ_1\bR_1$ is the reduced QR decomposition of $\bX$). Then the result of Equation~\eqref{equation:qr-for-ls-1} agrees with Theorem~\ref{theorem:qr-for-ls}.
%
%In conclusion, through QR decomposition, we directly obtain the least squares result, establishing the content of Theorem~\ref{theorem:qr-for-ls}. Additionally, we indirectly validate the OLS result from calculus using QR decomposition. The consistency between these two approaches is evident.

\index{QR decomposition}
\index{Gram-Schmidt}
\index{CGS}

\subsection{Gram-Schmidt QR}\label{section:qr-gram-compute}
Although the Householder (or Givens; see, for example, \citet{lu2021numerical} for more details) algorithm is more commonly used to compute the QR decomposition, we also present the Gram-Schmidt process. This method serves as the foundation for the elliptic MGS approach used in solving generalized least squares problems; see Section~\ref{section:gls_ellipmgs} for further details.

\subsection{Classical Gram-Schmidt (CGS) Process}

We express the reduced QR decomposition in the form $\bX = \bQ\bR$, where $\bQ\in \real^{n\times p}$ and $\bR\in \real^{p\times p}$.
The semi-orthogonal matrix $\bQ$ can be computed efficiently using the \textit{Gram-Schmidt process}.
Extending the concept from  Equation~\eqref{equation:gram-schdt-eq2} to the $k$-th term, we obtain
$$
\begin{aligned}
\bx_k &= \sum_{i=1}^{k-1}(\bq_i^\top\bx_k)\bq_i + \bx_k^\perp 
= \sum_{i=1}^{k-1}(\bq_i^\top\bx_k)\bq_i + \normtwo{\bx_k^\perp}\cdot \bq_k, 
\end{aligned}
$$
indicating that we can gradually orthonormalize $\bX$ to obtain an orthonormal set $\bQ=[\bq_1, \bq_2, \ldots, \bq_p]$ by 
\begin{equation}\label{equation:qr-gsp-equation}
\left\{
\begin{aligned}
r_{ik} &= \bq_i^\top\bx_k, \,\,\,\,\forall\, i \in \{1,2,\ldots, k-1\};\\ 
\bx_k^\perp&\triangleq \bx_k-\sum_{i=1}^{k-1}r_{ik}\bq_i;\\
r_{kk} &= \normtwo{\bx_k^\perp};\\
\bq_k &= \bx_k^\perp/r_{kk}.
\end{aligned}
\right.
\end{equation}
The procedure is summarized  in Algorithm~\ref{alg:reduced-qr}.
\begin{algorithm}[h] 
\caption{Reduced QR Decomposition via Gram-Schmidt Process} 
\label{alg:reduced-qr} 
\begin{algorithmic}[1] 
\Require Matrix $\bX$ has linearly independent columns with size $n\times p $ and $n\geq p$; 
\For{$k=1$ to $p$} \Comment{compute the $k$-th column of $\bQ,\bR$}
\For{$i=1$ to $k-1$}
\State $r_{ik} \leftarrow\bq_i^\top\bx_k$;              % \Comment{entry ($i,k$) of $\bR$, $2m-1$ flops}
\EndFor                                                 % \Comment{all $k-1$ iterations: $(k-1)(2m-1)$ flops}
\State $\bx_k^\perp\leftarrow  \bx_k-\sum_{i=1}^{k-1}r_{ik}\bq_i$;     %\Comment{$2m(k-1)$ flops}
\State $r_{kk} \leftarrow \normtwo{\bx_k^\perp}$;       % \Comment{main diagonal of $\bR$, $2m$ flops}
\State $\bq_k \leftarrow \bx_k^\perp/r_{kk}$;           %\Comment{$m$ flops}
\EndFor
\State Output $\bQ=[\bq_1, \ldots, \bq_p]$ and $\bR$ with entry $(i,k)$ being $r_{ik}$.
\end{algorithmic} 
\end{algorithm}

\begin{theoremHigh}[Algorithm complexity: reduced QR via Gram-Schmidt \citep{lu2021numerical}]\label{theorem:qr-reduced}
Algorithm~\ref{alg:reduced-qr} requires $\sim 2np^2$ flops to compute the reduced QR decomposition of an $n\times p$ matrix with linearly independent columns and $n\geq p$.
\end{theoremHigh}

\index{Orthogonal projection}
\index{Projection matrix}
\index{Projector}
\subsection*{{\textbf{Orthogonal Projection: Preliminary for MGS}}}
Upon revisiting  Equation~\eqref{equation:qr-gsp-equation}, i.e., step 2 to step 6 in Algorithm~\ref{alg:reduced-qr}, we observe that the first two equalities imply that
\begin{equation}\label{equation:qr-gsp-equation2}
\left.
\begin{aligned}
r_{ik} &= \bq_i^\top\bx_k, \,\forall\, i \in \{1,2,\ldots, k-1\}\\ 
\bx_k^\perp&\triangleq \bx_k-\sum_{i=1}^{k-1}r_{ik}\bq_i\\
\end{aligned}
\right\}
\rightarrow 
\bx_k^\perp= \bx_k- \bQ_{k-1}\bQ_{k-1}^\top \bx_k=(\bI-\bQ_{k-1}\bQ_{k-1}^\top )\bx_k,
\end{equation}
where $\bQ_{k-1}\triangleq[\bq_1,\bq_2,\ldots, \bq_{k-1}]$. This implies that $\bq_k$ can be obtained by 
$$
\bq_k = \frac{\bx_k^\perp}{\normtwo{\bx_k^\perp}} = \frac{(\bI-\bQ_{k-1}\bQ_{k-1}^\top )\bx_k}{\normtwo{(\bI-\bQ_{k-1}\bQ_{k-1}^\top )\bx_k}}.
$$
The matrix $(\bI-\bQ_{k-1}\bQ_{k-1}^\top )$ in the above equation is known as an \textit{orthogonal projection matrix} that  projects $\bx_k$ \textbf{along} the column space of $\bQ_{k-1}$,
i.e., it projects a vector so that the projected vector is perpendicular to the column space of $\bQ_{k-1}$; see Section~\ref{section:by-geometry-hat-matrix}. 
Consequently, $\bx_k^\perp$ or $\bq_k$ calculated in this way will be orthogonal to $\cspace(\bQ_{k-1})$, residing in the null space of $\bQ_{k-1}^\top$, i.e., the space of $\nspace(\bQ_{k-1}^\top)$ by the fundamental theorem of linear algebra (Theorem~\ref{theorem:fundamental-linear-algebra}). 
\index{Fundamental theorem}

\index{Complementary projector}
Let $\bP_1\triangleq(\bI-\bQ_{k-1}\bQ_{k-1}^\top )$. It can be shown  that $\bP_1=(\bI-\bQ_{k-1}\bQ_{k-1}^\top )$ is an orthogonal projection matrix such that $\bP_1\bv$ will project the vector $\bv$ onto the null space of $\bQ_{k-1}^\top$. 
Additionally, let $\bP_2\triangleq\bQ_{k-1}\bQ_{k-1}^\top$; then, $\bP_2$ is also an orthogonal projection matrix such that $\bP_2\bv$ will project the vector $\bv$ onto the column space of $\bQ_{k-1}$  (Proposition~\ref{proposition:orthogonal-projection}; $\bP_2$ is called a \textit{complementary projector} of $\bP_1$, vice versa).
Therefore, we conclude the presence of two orthogonal projections:
$$
\left\{
\begin{aligned}
\bP_1&=\bI-\bP_2: &\gap& \text{project onto $\nspace(\bQ_{k-1}^\top)$, \textbf{along} the column space of $\bQ_{k-1}$;} \\
\bP_2&=\bQ_{k-1}\bQ_{k-1}^\top: &\gap& \text{project onto $\cspace(\bQ_{k-1})$, \textbf{onto} the column space of $\bQ_{k-1}$} .
\end{aligned}
\right.
$$

An additional result to note arises when the columns of $\bQ_{k-1}$ are mutually orthonormal. In this case, we can observe the following decomposition:
\begin{equation}\label{equation:qr-orthogonal-equality}
%\boxed{
\bP_1 = \bI - \bQ_{k-1}\bQ_{k-1}^\top = (\bI-\bq_1\bq_1^\top)(\bI-\bq_2\bq_2^\top)\ldots (\bI-\bq_{k-1}\bq_{k-1}^\top),
%}
\end{equation}
where $\bQ_{k-1}=[\bq_1,\bq_2,\ldots, \bq_{k-1}]$ and each $(\bI-\bq_i\bq_i^\top)$ serves to project a vector onto the perpendicular space of $\bq_i$.

\index{MGS}
\subsection*{{\textbf{Modified Gram-Schmidt (MGS) Process}}}
To emphasize the modified Gram-Schmidt process and to make a connection to the equivalent projection in Equation~\eqref{equation:qr-orthogonal-equality}, we begin by illustrating a lemma that presents an alternative approach for obtaining the entries in the upper triangular matrix $\bR$ of the QR decomposition.
\begin{lemma}[Modified Gram-Schmidt]
Given a vector set $[\bx_1, \bx_2, \ldots,\bx_{k-1}, \bx_k$], where the first $k-1$ column are spanned by $k-1$ orthonormal vectors $[\bq_1, \bq_2, \ldots, \bq_{k-1}]$:
$$
\cspace([\bx_1, \bx_2, \ldots,\bx_{i}]) = \cspace([\bq_1, \bq_2, \ldots, \bq_{i}]), \gapforall \forall\, i\in \{1,2,\ldots, k-1\}. 
$$
Therefore, $r_{ik} = \bq_i^\top\bx_k$ represents the magnitude of the projection of $\bx_k$ on the vector $\bq_i$ (since $\bq_i$ is of unit length). Then it follows that 
$$
\begin{aligned}
\bq_i^\top\bx_k &= \bq_i^\top (\bx_k \underbrace{- r_{1k}\bq_1 - r_{2k}\bq_2 - \ldots - r_{i-1,k}\bq_{i-1}}_{\text{orthogonal to $\bq_i$}})\\
&= \bq_i^\top \big(\bx_k - \sum_{j=1}^{i-1}r_{jk}\bq_j \big), \gap \forall\, i\in \{1,2,\ldots, k-1\}.
\end{aligned}
$$
This can be easily verified since $\bq_i$ is orthonormal to $\{\bq_1, \bq_2, \ldots, \bq_{i-1}\}$. This observation implies another update for the $k$-th column of $\bR$. 
\end{lemma}
The lemma above reveals a second algorithm to compute the reduced QR decomposition of a matrix, as shown in Algorithm~\ref{alg:qr-mgs-right} of which the algorithm on the left is identical to Algorithm~\ref{alg:reduced-qr} (with slight modifications) to highlight the differences.

\noindent
\begin{minipage}[t]{0.495\linewidth}
\begin{algorithm}[H] 
\caption{CGS (=Algorithm~\ref{alg:reduced-qr})} 
\label{alg:qr-mgs-left}
\begin{algorithmic}[1] 
\Require $\bX\in \real^{n\times p}$ with full column rank;
\For{$k=1$ to $p$} 
\State $\bx_k^\perp\leftarrow\bx_k$;
\For{$i=1$ to $k-1$}
\State $\boxed{r_{ik} \leftarrow\bq_i^\top\bx_k}$;
\State $\bx_k^\perp\leftarrow \bx_k^\perp-r_{ik}\bq_i$; \,\,($\dagger$)
\EndFor 
\State $r_{kk} \leftarrow \normtwo{\bx_k^\perp}$; 
\State $\bq_k \leftarrow \bx_k^\perp/r_{kk}$; 
\EndFor 
\end{algorithmic} 
\end{algorithm}
\end{minipage}%
\hfil 
\begin{minipage}[t]{0.495\linewidth}
\begin{algorithm}[H] 
\caption{MGS}
\label{alg:qr-mgs-right}
\begin{algorithmic}[1] 
\Require $\bX\in \real^{n\times p}$ with full column rank;
\For{$k=1$ to $p$} 
\State $\bx_k^\perp\leftarrow\bx_k$;
\For{$i=1$ to $k-1$}
\State $\boxed{r_{ik} \leftarrow\bq_i^\top\textcolor{mylightbluetext}{\bx_k^\perp}}$;
\State $\bx_k^\perp\leftarrow \bx_k^\perp-r_{ik}\bq_i$; \,\,($\ast$)
\EndFor 
\State $r_{kk} \leftarrow \normtwo{\bx_k^\perp}$; 
\State $\bq_k \leftarrow \bx_k^\perp/r_{kk}$; 
\EndFor 
\end{algorithmic} 
\end{algorithm}
\end{minipage}
\index{Modified Gram–Schmidt}

The process described above is referred to as the \textit{modified Gram-Schmidt (MGS) process}, whereas the previous one is also known as the \textit{classical Gram-Schmidt (CGS) process}. In theory, both CGS and MGS are equivalent in the
sense that they compute exactly the same QR decompositions when exact arithmetic is employed. However, in practice, with the presence of round-off errors, the orthonormal columns of $\bQ$ computed by MGS tend to be ``more orthonormal" than those computed by CGS. 

To see the equivalence of the above two algorithms,
we note that the equality ($\dagger$) in Algorithm~\ref{alg:qr-mgs-left} is equivalent to 
\begin{equation}\label{equation:qr-orthogonal-equality2}
\bx_k^\perp = \bx_k - (\bq_1^\top\bx_k)\bq_1 - (\bq_2^\top\bx_k)\bq_2-\ldots -(\bq_{k-1}^\top\bx_k)\bq_{k-1} = (\bI - \bQ_{k-1}\bQ_{k-1}^\top) \bx_k.
\end{equation}
And
the equality ($\ast$) in Algorithm~\ref{alg:qr-mgs-right} can be reformulated as (via the step 4 and step 5 of the algorithm)
$$
\begin{aligned}
\bx_k^\perp &:= \bx_k^\perp-r_{ik}\bq_i
= \bx_k^\perp-(\bq_i^\top\bx_k^\perp)\bq_i
=\bx_k^\perp-\bq_i\bq_i^\top\bx_k^\perp
=(\bI-\bq_i\bq_i^\top)\bx_k^\perp.
\end{aligned}
$$
That is, $\bx_k^\perp$ will be updated by 
\begin{equation}\label{equation:qr-orthogonal-equality3}
\bx_k^\perp=
\left\{(\bI-\bq_{k-1}\bq_{k-1}^\top)\ldots\left[(\bI-\bq_2\bq_2^\top)\left((\bI-\bq_1\bq_1^\top) \bx_k\right)\right]\right\},
\end{equation}
where the nested parentheses in MGS denote the computation order.
The comparison of \eqref{equation:qr-orthogonal-equality2} and \eqref{equation:qr-orthogonal-equality3} matches the orthogonal projection matrix equality in Equation~\eqref{equation:qr-orthogonal-equality} that
$$
\begin{aligned}
\bP_1 &= \bI-\bQ_{k-1}\bQ_{k-1}^\top
=(\bI-\bq_1\bq_1^\top)(\bI-\bq_2\bq_2^\top)\ldots (\bI-\bq_{k-1}\bq_{k-1}^\top)
=\prod_{i=1}^{k-1}(\bI-\bq_i\bq_i^\top),
\end{aligned}
$$
where $\bQ_{k-1}=[\bq_1,\bq_2,\ldots, \bq_{k-1}]$.

\begin{figure}[H]
\centering  
\vspace{-0.35cm} 
\subfigtopskip=2pt 
\subfigbottomskip=2pt 
\subfigcapskip=-5pt 
\subfigure[CGS, step 1: \textcolor{mylightbluetext}{blue} vector; step 2: \textcolor{mydarkgreen}{green} vector; step 3: \textcolor{mydarkpurple}{purple} vector.]{\label{fig:projection-mgs-demons-cgs}
\includegraphics[width=0.4\linewidth]{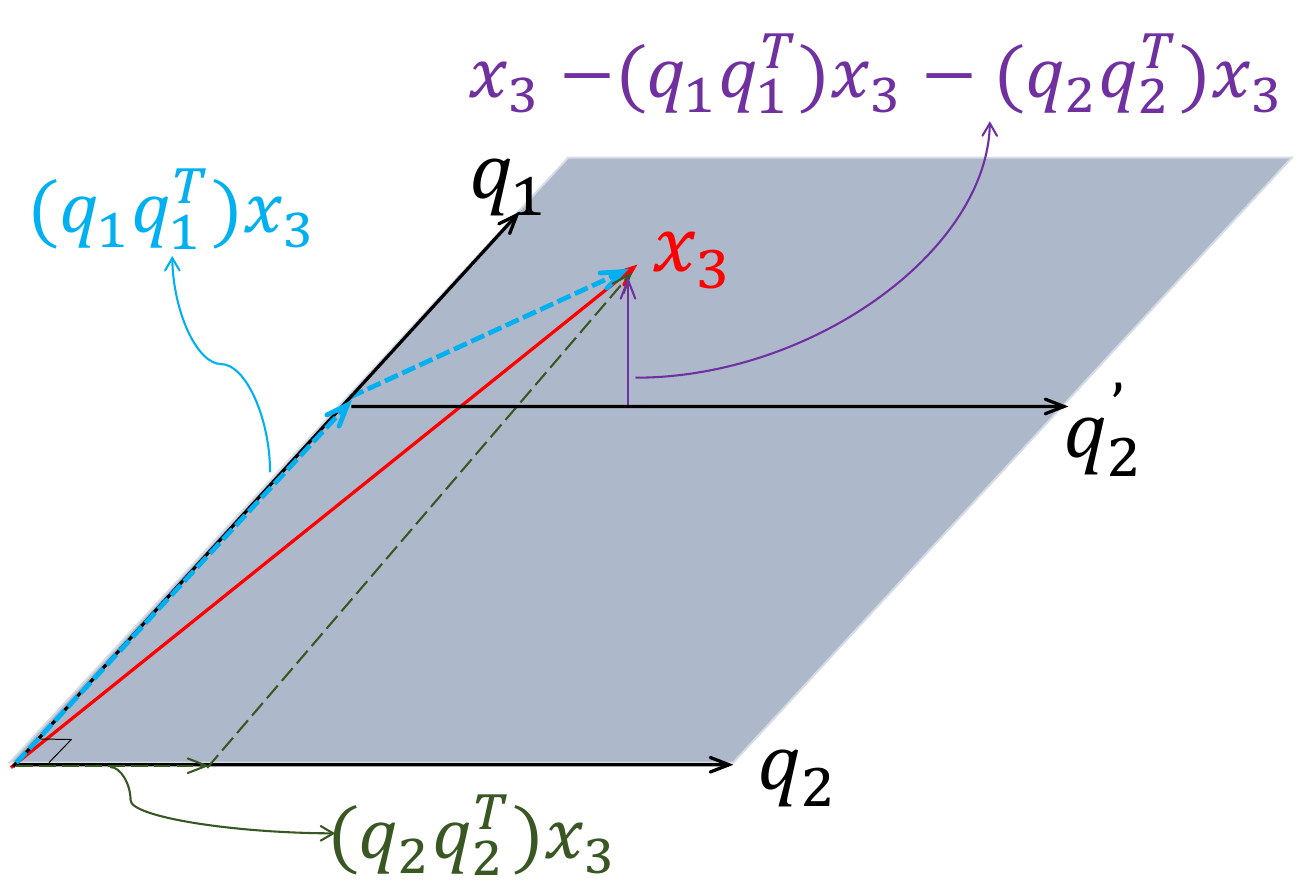}}
\quad 
\subfigure[MGS, step 1: \textcolor{mylightbluetext}{blue} vector; step 2: \textcolor{mydarkpurple}{purple} vector.]{\label{fig:projection-mgs-demons-mgs}
\includegraphics[width=0.4\linewidth]{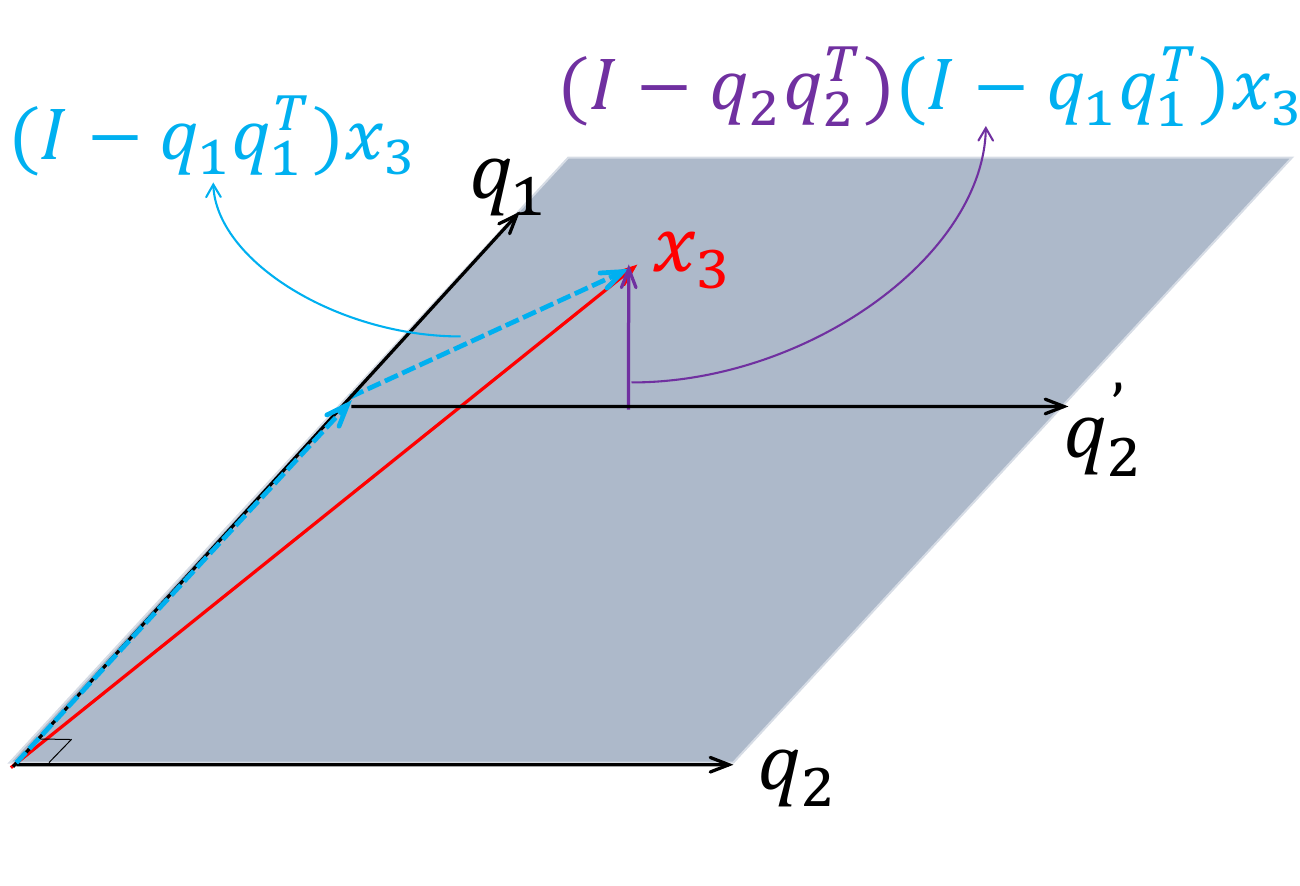}}
\caption{CGS vs MGS in three-dimensional space, where $\bq_2^\prime$ is parallel to $\bq_2$ so that projecting onto $\bq_2$ is equivalent to projecting onto $\bq_2^\prime$.}
\label{fig:projection-mgs-demons-3d}
\end{figure}

\paragrapharrow{What's the difference?}
Taking a three-column matrix $\bX=[\bx_1, \bx_2, \bx_3]$ as an example. Suppose we have already computed $\{\bq_1, \bq_2\}$ such that $\spn\{\bq_1, \bq_2\}=\spn\{\bx_1, \bx_2\}$, and we want to proceed to compute  $\bq_3$.

In the CGS algorithm, the orthogonalization of column $\bx_3$ against columns $\{\bq_1, \bq_2\}$ is achieved by projecting the original column $\bx_3$ of $\bX$ onto $\{\bq_1, \bq_2\}$, respectively, followed by subtracting these projections at once:
\begin{equation}\label{equation:cgs-3d-exmp}
\left\{
\begin{aligned}
\bx_3^\perp &= \bx_3 - (\bq_1^\top\bx_3)\bq_1 - (\bq_2^\top\bx_3)\bq_2\\
&= \bx_3 - (\bq_1\bq_1^\top)\bx_3 - \boxed{(\bq_2\bq_2^\top)\textcolor{mylightbluetext}{\bx_3}};\\
\bq_3 &=  {\bx_3^\perp}/{\normtwo{\bx_3^\perp}},
\end{aligned}
\right.
\end{equation}
as shown in Figure~\ref{fig:projection-mgs-demons-cgs}.

In the MGS algorithm, on the other hand, the components along each $\{\bq_1, \bq_2\}$ are immediately subtracted out of the rest of the column $\bx_3$ as soon as the vectors  $\{\bq_1, \bq_2\}$ are computed. 
Therefore, the orthogonalization of column $\bx_3$ against $\{\bq_1, \bq_2\}$ is not performed by projecting the original column $\bx_3$ against $\{\bq_1, \bq_2\}$ as it is in CGS, but rather against a vector obtained by subtracting from that column $\bx_3$ of $\bX$ the components in the direction of $\bq_1, \bq_2$ successively. This is important because the error components of $\bq_3$ in $\spn\{\bq_1, \bq_2\}$ will be smaller (we will  discuss this further in the following paragraphs).

More precisely, in the MGS algorithm, the orthogonalization of column $\bx_3$ against $\bq_1$ is performed by subtracting the component of $\bq_1$ from the vector $\bx_3$:
$$
\bx_3^{(1) }=  (\bI-\bq_1\bq_1^\top)\bx_3 = \bx_3 - (\bq_1\bq_1^\top)\bx_3,
$$
where $\bx_3^{(1) }$ represents the component of $\bx_3$ that lies in a space perpendicular to $\bq_1$. And the further step is performed by 
\begin{equation}\label{equation:mgs-3d-exmp}
\begin{aligned}
\bx_3^{(2) }=  (\bI-\bq_2\bq_2^\top)\bx_3^{(1) }&=\bx_3^{(1) }-(\bq_2\bq_2^\top)\bx_3^{(1) }\\
&=\bx_3 - (\bq_1\bq_1^\top)\bx_3-\boxed{(\bq_2\bq_2^\top)\textcolor{mylightbluetext}{\bx_3^{(1) }}},
\end{aligned}
\end{equation}
where $\bx_3^{(2) }$ represents the component of $\bx_3^{(1) }$ that lies in a space perpendicular to $\bq_2$. And we highlight the difference from the CGS algorithm in Equation~\eqref{equation:cgs-3d-exmp} using \textcolor{mylightbluetext}{\boxed{blue}} text. 
As a result, $\bx_3^{(2) }$ corresponds to the component of $\bx_3$ that lies in the space perpendicular to $\{\bq_1, \bq_2\}$, as shown in Figure~\ref{fig:projection-mgs-demons-mgs}. 

\index{Cancellation}
\index{Catastrophic cancellation}
\subsection*{{\textbf{Main Difference and Catastrophic Cancellation}}}
The key difference between the CGS and MGS processes lies in the fact that $\bx_3$ can generally have large components in $\spn\{\bq_1, \bq_2\}$. 
In such cases, one starts with large
values and ends up with small values that yields significant relative errors in them. 
This phenomenon is commonly referred to as \textit{catastrophic cancellation}. 
Whereas, $\bx_3^{(1) }$ lies in the direction perpendicular to $\bq_1$ and carries only a small ``error" component in the direction of $\bq_1$. Comparing the \fbox{boxed} terms in Equations~\eqref{equation:cgs-3d-exmp} and \eqref{equation:mgs-3d-exmp}, it is not hard to see that $(\bq_2\bq_2^\top)\bx_3^{(1) }$ in Equation~\eqref{equation:mgs-3d-exmp} is more accurate based on the above argument. And thus, because of the much smaller error in this projection factor, the MGS introduces smaller orthogonalization error at each subtraction step compared to the CGS method. In fact, it can be shown that the final $\bQ$ obtained in the CGS satisfies
$$
\normtwo{\bI-\bQ\bQ^\top} \leq \mathcalO(\epsilon \kappa^2(\bX)),
$$
where $\kappa(\bX)$ is a value larger than 1 determined by $\bX$.
Whereas, in MGS, the error satisfies
$$
\normtwo{\bI-\bQ\bQ^\top} \leq \mathcalO(\epsilon \kappa(\bX)).
$$
That is, the $\bQ$ obtained via MGS is ``more orthogonal."
Therefore, we summarize the difference between the CGS and MGS processes for obtaining $\bq_k$ from the $k$-th column $\bx_k$ of $\bX$, given the orthonormalized vectors $\{\bq_1, \bq_2, \ldots, \bq_{k-1}\}$:
$$
\begin{aligned}
\text{(CGS)}: &\, \text{obtain $\bq_k$ by normalizing $\bx_k^\perp=(\bI-\bQ_{k-1}\bQ_{k-1}^\top)\bx_k$;} \\
\text{(MGS)}: &\, \text{obtain $\bq_k$ by normalizing  $\bx_k^\perp=\big\{(\bI-\bq_{k-1}\bq_{k-1}^\top)\ldots\big[(\bI-\bq_2\bq_2^\top)\big((\bI-\bq_1\bq_1^\top) \bx_k\big)\big]\big\}$.} 
\end{aligned}
$$

\paragrapharrow{Comparison with the Householder algorithm.}
Although both methods have their advantages, in practice,  MGS  usually outperforms  CGS; see examples in \citet{lu2021numerical}.  
However,  MGS can still fall victim to the \textit{catastrophic cancellation} problem. Suppose, in iteration $k$ of the MGS Algorithm~\ref{alg:qr-mgs-right}, $\bx_k$ is almost in the span of $\{\bq_1, \bq_2, \ldots, \bq_{k-1}\}$. This will result in that $\bx_k^\perp$ has only a small component that is perpendicular to $\spn\{\bq_1, \bq_2, \ldots, \bq_{k-1}\}$, whereas the ``error" component in the $\spn\{\bq_1, \bq_2, \ldots, \bq_{k-1}\}$ will be amplified, resulting in $\bQ$ being less orthonormal. 
As mentioned earlier, both the CGS and MGS methods suffer from the same main disadvantage: they obtain the orthogonal matrix $\bQ$ through the upper triangular matrix $\bR$. Specifically, for $\bX\in \real^{n\times p}$, $\bQ$ can be obtained by using the following equation:
$$
\bQ=\bX\underbrace{\bR_1^{-1}\bR_2^{-1}\ldots \bR_p^{-1}}_{\bR^{-1}}.
$$
In this case, the Householder algorithm---that will be introduced in the sequel---finds a successive set of orthogonal matrices $\{\bQ_1, \bQ_2, \ldots, \bQ_l\}$ such that $\bQ_l\ldots\bQ_2\bQ_1\bX$ is triangularized, then $\bQ=(\bQ_l\ldots\bQ_2\bQ_1)^\top$ will be ``more" orthogonal than that in the CGS or MGS method since the condition numbers for the orthogonal matrices are all 1.

\subsection*{{\textbf{Row-Wise MGS, Recursive Algorithm and Other Issues}}}
The algorithms presented  in Algorithm~\ref{alg:qr-mgs-left} and \ref{alg:qr-mgs-right} are used to calculate the entries of the upper triangular matrix $\bR$ in an element-wise and column-by-column manner. 
Suppose $\bX$ has column partition $\bX=[\bx_1, \bX_2]$, where $\bX_2=[\bx_2, \bx_3, \ldots, \bx_p]\in \real^{n\times (p-1)}$. Notice in the CGS Algorithm~\ref{alg:qr-mgs-left}, the first row of $\bR$ can be obtained by 
$$
\left.
\begin{aligned}
r_{11} &= \normtwo{\bx_1}\\
r_{1k} & = \bq_1^\top\bx_k, \,\,\, \forall\, k\in\{2,3,\ldots, p\}
\end{aligned}\right\}
\leadto
\left\{
\begin{aligned}
r_{11} &= \normtwo{\bx_1}\\
\br_{12}^\top & = \bq_1^\top\bX_2, \,\,\, \br_{12}=[r_{12}, r_{13}, \ldots, r_{1p}].
\end{aligned}\right.
$$
Therefore, the QR decomposition of $\bX$ is given by 
$$
\bX = 
\begin{bmatrix}
\bx_1 & \bX_2
\end{bmatrix}=
\begin{bmatrix}
\bq_1 & \bQ_2
\end{bmatrix}
\begin{bmatrix}
r_{11} & \br_{12}^\top\\
\bzero & \bR_{22}
\end{bmatrix}
=
\begin{bmatrix}
r_{11}\bq_1 & \bq_1\br_{12}^\top +\bQ_2\bR_{22}
\end{bmatrix},
$$
where the matrix $\bQ_2\in \real^{n\times (p-1)}$ consists of mutually orthonormal columns   and $\bR_{22}\in \real^{(p-1)\times (p-1)}$ is upper triangular. 
Consequently, $\bQ_2\bR_{22} $ represents the reduced QR decomposition of $\bX_2-\bq_1\br_{12}^\top$, which reveals a recursive algorithm for computing the reduced QR decomposition of $\bX$. 
This approach is equivalent to the MGS method that subtracts each component in the span of $\{\bq_1, \bq_2, \ldots, \bq_{k-1}\}$ when computing column $k$ of $\bQ$ (i.e., equality ($*$) in Algorithm~\ref{alg:qr-mgs-right}). The process is described in Algorithm~\ref{alg:qr-mgs-fulll-rowwise-recursive}.
\begin{algorithm}[H] 
\caption{MGS (\textcolor{mylightbluetext}{Row-Wise and Recursively})=Algorithm~\ref{alg:qr-mgs-right}}
\label{alg:qr-mgs-fulll-rowwise-recursive}
\begin{algorithmic}[1] 
\Require $\bX\in \real^{n\times p}$ with full column rank;
\For{$k=1$ to $p$}  \Comment{i.e., compute $k$-th column of $\bQ$ and $k$-th row of $\bR$}
\State $\bx_1\leftarrow\bX[:,1]$; \Comment{$1$-st column of $\bX\in \real^{n\times (p-k+1)}$}
\State $r_{kk}\leftarrow\normtwo{\bx_1}$;\Comment{$\bx_1\in \real^{n\times 1}$}
\State $\bq_k \leftarrow \bx_1/r_{kk}$;
\State $\bX_2\leftarrow\bX[:,2:p]\in \real^{n\times (p-k)}$; \Comment{2-nd to $p$-th column of $\bX$}
\State $\br_{k2}^\top\leftarrow\bq_k^\top\bX_2$; \Comment{$\br_{k2}^\top\in \real^{1\times (p-k)}$}
\State $\bX\leftarrow\bX_2-\bq_k\br_{k2}^\top$; \Comment{$\bX \in \real^{n\times (p-k)}$}
\EndFor 
\State Output $\bQ=[\bq_1, \ldots, \bq_p]$ and $\bR$ with entry $(i,k)$ being $r_{ik}$.
\end{algorithmic} 
\end{algorithm}

More compactly, Algorithm~\ref{alg:qr-mgs-fulll-rowwise-recursive} can be equivalently stated as Algorithm~\ref{alg:mgs_rowwise-recursive_comp}.
\begin{algorithm}[H] 
\caption{MGS (\textcolor{mylightbluetext}{Row-Wise and Recursively})=Algorithm~\ref{alg:qr-mgs-right}=Algorithm~\ref{alg:qr-mgs-fulll-rowwise-recursive}}
\label{alg:mgs_rowwise-recursive_comp}
\begin{algorithmic}[1] 
\Require $\bX\in \real^{n\times p}$ with full column rank;
\For{$k=1$ to $p$}  \Comment{i.e., compute $k$-th column of $\bQ$ and $k$-th row of $\bR$}
\State $\bq_k\leftarrow\bX[:,k]/\normtwo{\bX[:,k]}$; \Comment{Normalize $k$-th column of $\bX\in \real^{n\times p}$}
\State $\br_k^\top \leftarrow \bq_k^\top \bX$; \Comment{$\br_i^\top\in\real^{1\times p}$, $k$-th row of $\bR$}
\State $\bX \leftarrow\bX-\bq\br^\top$; \Comment{MGS step}
\EndFor 
\State Output $\bQ=[\bq_1, \ldots, \bq_p]$ and $\bR$ with entry $k$-th row being $\br_k^\top$.
\end{algorithmic} 
\end{algorithm}

To enhance the orthogonality of the ${\bq_i}$'s and improve numerical accuracy, an additional re-orthonormalization step can be performed. This step becomes necessary as the basis vectors generated tend to lose their orthonormality during the process. 
The re-orthonormalization steps are highlighted in blue in Algorithm~\ref{alg:mgs_recursive_comp_moreortho}.
\begin{algorithm}[H] 
\caption{MGS (\textcolor{mylightbluetext}{Row-Wise, Re-Orthonormalization} based on Algorithm~\ref{alg:mgs_rowwise-recursive_comp})}
\label{alg:mgs_recursive_comp_moreortho}
\begin{algorithmic}[1] 
\Require $\bX\in \real^{n\times p}$ with full column rank;
\For{$k=1$ to $p$}  \Comment{i.e., compute $k$-th column of $\bQ$ and $k$-th row of $\bR$}
\State $\bq_k\leftarrow\bX[:,k]/\normtwo{\bX[:,k]}$; \Comment{Normalize $k$-th column of $\bX\in \real^{n\times p}$}
\State \textcolor{mylightbluetext}{$\bq\leftarrow\bq - \bQ_{k-1}\bQ_{k-1}^\top \bq$}; \Comment{$\bQ_{k-1} =[\bq_1,\bq_2, \ldots, \bq_{k-1}]\in\real^{n\times (k-1)} $}
\State \textcolor{mylightbluetext}{$\bq \leftarrow \bq/\normtwo{\bq}$};
\State $\br_k^\top \leftarrow \bq_k^\top \bX$; \Comment{$\br_i^\top\in\real^{1\times p}$, $k$-th row of $\bR$}
\State $\bX \leftarrow\bX-\bq\br^\top$; \Comment{MGS step}
\EndFor 
\State Output $\bQ=[\bq_1, \ldots, \bq_p]$ and $\bR$ with entry $k$-th row being $\br_k^\top$.
\end{algorithmic} 
\end{algorithm}

%\subsection*{{\textbf{Computing Full QR Decomposition  via  Gram-Schmidt Process}}}\label{section:silentcolu_qrdecomp}
\paragrapharrow{Full QR decomposition.}
A full QR decomposition of an $n\times p$ matrix with linearly independent columns extends the process by appending additional $n-p$ orthonormal columns to $\bQ$, thereby transforming it into an $n\times n$ orthogonal matrix. Simultaneously, $\bR$ is augmented with rows of zeros to attain an $n\times p$ upper triangular matrix.
We refer to the additional columns in $\bQ$ as \textbf{silent columns} and the additional rows in $\bR$ as \textbf{silent rows}. The comparison between the reduced   and the full QR decompositions is shown in Figure~\ref{fig:qr-comparison}.

%\subsection*{{\textbf{Dependent Columns}}}\label{section:dependent-gram-schmidt-process}
\paragrapharrow{Dependent columns.}
In our earlier discussions, we assumed that the matrix $\bX$ has linearly independent columns. While this condition simplifies our analysis, it is not strictly required for all scenarios. 
Consider step $k$ Algorithm~\ref{alg:reduced-qr}, where $\bx_k$ lies in the plane spanned by $\bq_1, \bq_2, \ldots, \bq_{k-1}$ (which is equivalent to the space spanned by $\bx_1, \bx_2, \ldots, \bx_{k-1}$), indicating that the vectors $\bx_1, \bx_2, \ldots, \bx_k$ are dependent. 
Then $r_{kk}$ will be zero and $\bq_k$ cannot be determined  due to division by zero.
In such cases, we can arbitrarily choose $\bq_k$ to be any normalized vector that is orthogonal to $\cspace([\bq_1, \bq_2, \ldots, \bq_{k-1}])$ and proceed with the Gram-Schmidt process. 
%Again, when dealing with a matrix $\bX$ that has dependent columns, we have both reduced and full QR decomposition algorithms. 
We reformulate the step $k$ in the algorithm as follows:
$$
\bq_k=\left\{
\begin{aligned}
&\big(\bx_k-\sum_{i=1}^{k-1}r_{ik}\bq_i\big)/r_{kk}, \qquad r_{ik}=\bq_i^\top\bx_k, r_{kk}=\normtwobig{\bx_k-\sum_{i=1}^{k-1}r_{ik}\bq_i}, &\mathrm{if\,} r_{kk}\neq0, \\
&\text{pick  one in\,}\cspace^{\bot}([\bq_1, \bq_2, \ldots, \bq_{k-1}]),\text{ and normalize},\qquad &\mathrm{if\,} r_{kk}=0.
\end{aligned}
\right.
$$

This idea can be further extended such that,  when $\bq_k$ does not exist, we simply skip the current step and add the silent columns at the end of the process.  Consequently, the QR decomposition of a matrix with dependent columns is not unique. 

Moreover, this insight also aids in determining the linear independence of a set of vectors. 
Whenever $r_{kk}$ in Algorithm~\ref{alg:reduced-qr} becomes zero, we report the vectors $\bx_1, \bx_2, \ldots, \bx_k$ are dependent and terminate the algorithm for ``\textit{independence checking}."

\subsection{LS by Householder QR}
\textit{Householder matrices}, also known as \textit{Householder reflectors}, which can reflect vectors, play a crucial role in numerical linear algebra for tasks such as solving linear systems, addressing least squares problems, and deriving Hessenberg forms \citep{golub2013matrix, lu2021numerical}. In this section, we present how Householder reflectors can be used to prove the existence of and compute the QR decomposition.

\subsection*{\textbf{{Computing QR via Householder}}}
\paragrapharrow{{Householder reflectors}.}
Let's begin with the formal definition of a Householder reflector, exploring its properties thereafter.
\begin{definition}[Householder reflector\index{Householder reflector}]\label{definition:householder-reflector}
Let $\bu \in \real^n$ be a unit vector ($\normtwo{\bu}=1$). 
The matrix $\bH = \bI - 2\bu\bu^\top$ is referred to as a \textit{Householder reflector}, a.k.a., a \textit{Householder transformation}. We call this $\bH$ the Householder reflector associated with the unit vector $\bu$, where the unit vector $\bu$ is also known as the \textit{Householder vector}. 
When a vector $\bx$ is multiplied by $\bH$, it undergoes reflection with respect to the hyperplane $\spn\{\bu\}^\perp$.

Note that if $\normtwo{\bu} \neq 1$, we can define the Householder reflector $\bH$ as $\bH = \bI - 2  \frac{\bu\bu^\top}{\bu^\top\bu}$.
\end{definition}
From the definition of the Householder reflector, we can derive the following corollary, which states that certain vectors remain unchanged under the action of the Householder reflector.
\begin{corollary}[Unreflected by Householder]
Let $\bu\in\real^n$ be given with $\normtwo{\bu}=1$, and  define the Householder reflector as $\bH=\bI-2\bu\bu^\top$. 
The Householder reflector leaves  any vector $\bv$  perpendicular to $\bu$ unchanged; that is, $\bH\bv=\bv$ if $\bu^\top\bv=0$. 
\end{corollary}
The proof is straightforward since $(\bI - 2\bu\bu^\top)\bv = \bv - 2\bu\bu^\top\bv=\bv$.

Suppose $\bu$ is a unit vector with $\normtwo{\bu}=1$, and let $\bv$ be a vector perpendicular to $\bu$. Then any vector $\bx$ in the plane can be decomposed into two components:
$$
\bx = \bx_{\bu} + \bx_{\bv},
$$ 
where the first component $\bx_{\bu}$ is parallel to $\bu$ and the second one $\bx_{\bv}$ is perpendicular to $\bu$ (i.e., parallel to $\bv$). 
Referring to Section~\ref{section:project-onto-a-vector} on vector projections, $\bx_{\bu}$ can be computed as $\bx_{\bu} = \frac{\bu\bu^\top}{\bu^\top\bu} \bx = \bu\bu^\top\bx$, representing  the projection of $\bx$ onto  $\bu$. 
Applying the Householder reflector associated with $\bu$ to the vector $\bx$, we obtain:
$$\bH\bx = (\bI - 2\bu\bu^\top)(\bx_{\bv} + \bx_{\bu}) = \bx_{\bv} -\bu\bu^\top \bx = \bx_{\bv} - \bx_{\bu},
$$ 
which means the Householder reflector transforms $\bx_{\bv} + \bx_{\bu}$ into $\bx_{\bv} - \bx_{\bu}$.
In other words, the space perpendicular to $\bu$ acts as a mirror, and any vector $\bx$ is reflected by the Householder reflector associated with $\bu$ (i.e., reflected by the hyperplane $\spn\{\bu\}^\perp$). 
The situation is illustrated in Figure~\ref{fig:householder}.

\begin{SCfigure}
\centering
\includegraphics[width=0.5\textwidth]{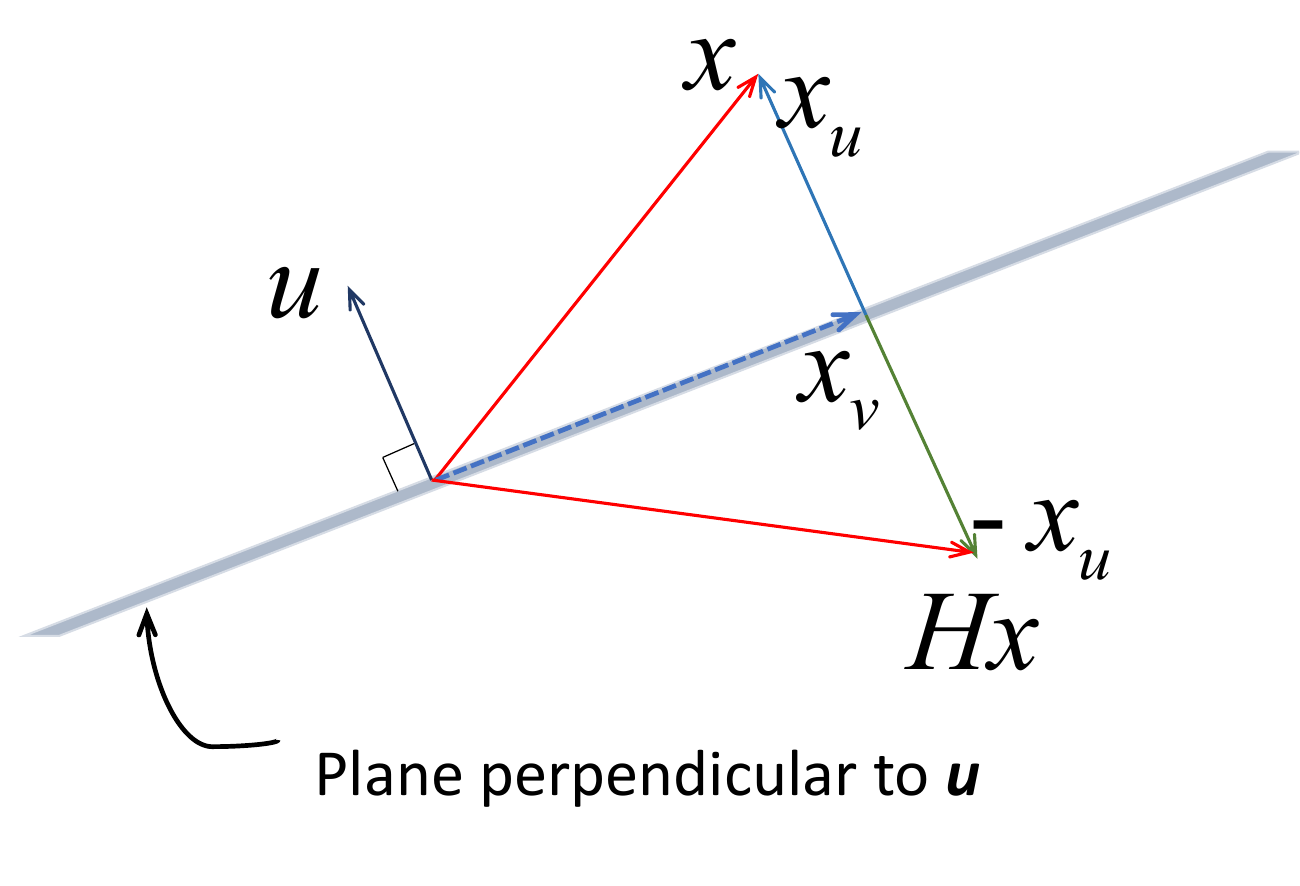}
\caption{Demonstration of the Householder reflector. The Householder reflector obtained by $\bH=\bI-2\bu\bu^\top$, where $\normtwo{\bu}=1$, will reflect a vector $\bx$ along the plane perpendicular to $\bu$. Specifically, it transforms $\bx=\bx_{\bv} + \bx_{\bu}$ into $\bx_{\bv} - \bx_{\bu}$.}
\label{fig:householder}
\end{SCfigure}

The previous discussion explains how to determine the reflected vector using the Householder reflector. 
However, an additional question arises: If we know in advance that two vectors are reflections of each other, how can we find the corresponding Householder reflector? 
The property is crucial for computing the QR decomposition, where we aim to reflect/transform a column into a specific form.
\begin{corollary}[Reflection theorem: finding the Householder reflector]\label{corollary:householder-reflect-finding}
Suppose $\bx$ is reflected to $\by$ by a Householder reflector, with $\normtwo{\bx} = \normtwo{\by}$. The (unique) Householder reflector can be obtained by
$$
\bH = \bI - 2 \bu\bu^\top, \text{ where } \bu = \frac{\bx-\by}{\normtwo{\bx-\by}}.
$$
\end{corollary}
\begin{proof}[of Corollary~\ref{corollary:householder-reflect-finding}]
Write out the equation, we have 
$$
\begin{aligned}
\bH\bx &= \bx - 2 \bu\bu^\top\bx =\bx - 2\frac{(\bx-\by)(\bx^\top-\by^\top)}{(\bx-\by)^\top(\bx-\by)} \bx
= \bx - (\bx-\by) = \by.
\end{aligned}
$$
Note that the condition $\normtwo{\bx} = \normtwo{\by}$ is required to prove the result.
\end{proof}

Householder reflectors are useful for setting a block of components of a given vector to zero. Particularly, it is often desirable to set the vector $\ba\in \real^n$ to  zero, except for the $i$-th element. 
In such cases, the Householder vector can be chosen as:
$$
\bu = \frac{\ba - r\be_i}{\normtwo{\ba - r\be_i}}, \qquad \text{where } r = \pm\normtwo{\ba},
$$
which is a reasonable Householder vector since $\normtwo{\ba} = \normtwo{r\be_i} = \abs{r}$. We carefully notice that when $r=\normtwo{\ba}$, $\ba$ is reflected to $\normtwo{\ba}\be_i$ via the Householder reflector $\bH = \bI - 2 \bu\bu^\top$; conversely, when $r=-\normtwo{\ba}$, $\ba$ is reflected to $-\normtwo{\ba}\be_i$.

Recalling from Section~\ref{section:qr-gram-compute}, we claimed the Householder method (or the Givens method)  utilizes a set of orthogonal matrices to triangularize the matrix, thereby obtaining the QR decomposition and achieving a higher level of orthogonality in this context. The Householder reflector serves as one such orthogonal matrix for this purpose. 
In the following remark, we present additional properties of the Householder reflector.
\begin{remark}[Householder properties]\label{remark:householder-propes}
A Householder reflector $\bH$ possesses the following properties:
\begin{itemize}
\item $\bH\bH = \bI$, i.e., reflecting a vector twice is equivalent to not reflecting it at all.

\item Symmetry: $\bH = \bH^\top$.

\item Orthogonality: $\bH^\top\bH = \bH\bH^\top = \bI$ such that the Householder reflector is an orthogonal matrix.

\item $\bH\bu = -\bu$, if $\bH = \bI - 2 \bu\bu^\top$.
\item Unit eigenvalues: the eigenvalue of $\bH$ is either $1$ or $-1$. Given an eigenpair $(\lambda, \bx)$ of $\bH$, it follows that $\normtwo{\bH\bx}=\normtwo{\lambda\bx}=\normtwo{\bx}$. Therefore, $\lambda=\pm 1$.
\item The determinant of a Householder reflector is $-1$.
\end{itemize}
\end{remark}

\paragrapharrow*{{Householder QR decomposition}.}
To reiterate, as discussed in the Gram-Schmidt section, QR decomposition involves using a triangular matrix to orthogonalize a matrix $\bX$. 
Building upon this concept, if we have a sequence of orthogonal matrices that can successively transform $\bX$ into an upper triangular form, we can also construct the QR decomposition.
In particular, suppose we have an orthogonal matrix $\bQ_1$ that introduces zeros into all entries of the first column of $\bX$ except for the element at position (1,1); 
and similarly, an orthogonal matrix $\bQ_2$ that introduces zeros into the second column except for the elements at positions (2,1) and (2,2); 
$\ldots$. Then we can also find the QR decomposition.
To achieve this zero introduction, we could reflect the columns of the matrix to a basis vector $\be_1$ whose entries are all zero except the first entry.

Let $\bX=[\bx_1, \bx_2, \ldots, \bx_p]\in \real^{n\times p}$ be the column partition of $\bX$, and define further
\begin{equation}\label{equation:qr-householder-to-chooose-r-numeraically}
r_1 \triangleq \normtwo{\bx_1},
\qquad \bu_1  \triangleq {\bx_1 - r_1 \be_1}, 
\qquad \bH_1  \triangleq \bI - \tau_1\bu_1\bu_1^\top, 
\qquad \tau_1 \triangleq \frac{1}{\normtwo{\bx_1 - r_1 \be_1}^2},
\end{equation}
where $\be_1$ here is the first unit basis in $\real^n$, i.e., $\be_1 = [1;0;0;\ldots;0]\in \real^n$.
% This Householder reflector indicates that we want to reflect $\bx_1$ to $r_1\be_1$.
Then,
\begin{equation}\label{equation:householder-qr-projection-step1}
\bH_1\bX = [\bH_1\bx_1, \bH_1\bx_2, \ldots, \bH_1\bx_p] 
\triangleq
\begin{bmatrix}
r_1 & \bR_{1,2:p} \\
\bzero&  \bB_2
\end{bmatrix},
\end{equation}
which reflects $\bx_1$ to $r_1\be_1$ and introduces zeros below the diagonal in the first column. We observe that the entries below $r_1$ become zero after this specific reflection. Notice that we reflect $\bx_1$ to $\normtwo{\bx_1}\be_1$, where both vectors have the same length, rather than reflect $\bx_1$ to $\be_1$ directly  to match the requirement stated in Corollary~\ref{corollary:householder-reflect-finding}. 

It should be noted that the choice of $r_1$ is \textbf{not unique}. To ensure \textbf{numerical stability}, it is desirable to set $r_1 =-\text{sign}(a_{11}) \normtwo{\bx_1}$, where $a_{11}$ represents the first component of $\bx_{1}$. 
Alternatively, one can also choose $r_1 =\text{sign}(a_{11}) \normtwo{\bx_1}$, as long as $\normtwo{\bx_1}$ is equal to $\normtwo{r_1\be_1}$. 
However, if we require the diagonal entries of the upper triangular matrix $\bR$ to be positive---so that the QR decomposition is unique---it becomes necessary to select a positive value for $r_1$.

Next, we can apply this process to $\bB_2$ in \eqref{equation:householder-qr-projection-step1} to transform the entries  below the entry (2,2) into zeros. 
Note that we do not apply this process to the entire matrix $\bH_1\bX$ but only to the submatrix $\bB_2$ in it because we have already introduced zeros in the first column, and reflecting again will reintroduce nonzero values back and destroy what have accomplished.

Suppose $\bB_2 = [\bb_2, \bb_3, \ldots, \bb_p]$ is the column partition of $\bB_2$, and define 
\begin{equation}\label{equation:qr-householder-to-chooose-r-numeraically2}
r_2 \triangleq \normtwo{\bb_2},
\quad \widetildebu_2 \triangleq {\bb_2 - r_2 \be_1},   
\quad \widetilde{\bH}_2 \triangleq \bI - \tau_2\widetildebu_2\widetildebu_2^\top, 
\quad \tau_2 = \frac{1}{\footnotesize\normtwo{\bb_2 - r_2 \be_1}^2},
\quad 
\bH_2 \triangleq
\begin{bmatrixfoot}
	1 & \bzero \\
	\bzero & \widetilde{\bH}_2
\end{bmatrixfoot}.
\end{equation}
In this context, $\be_1$ denotes the first unit basis in $\real^{n-1}$, and $\bH_2$ is also an orthogonal matrix since $\widetilde{\bH}_2$ is orthogonal.
Note that $\bH_2$ is also an Householder reflector since it can can be equivalently denoted as $\bH_2=\bI-2\bu_2\bu_2^\top$ with $\bu_2 = \begin{bmatrixfoot}
\bzero \\
\widetildebu_2 
\end{bmatrixfoot}$.
Applying $\widetilde{\bH_2}$ or $\bH_2$ yields
$$
\widetilde{\bH_2}\bB_2 = [\bH_2\bb_2, \bH_2\bb_3, \ldots, \bH_2\bb_p]
\triangleq
\begin{bmatrix}
r_2 & \bR_{2,3:p} \\
\bzero &\bC_3
\end{bmatrix},
$$
and
$$
\bH_2\bH_1\bX = [\bH_2\bH_1\bx_1, \bH_2\bH_1\bx_2, \ldots, \bH_2\bH_1\bx_p]
\triangleq
\begin{bmatrixfoot}
r_1 & r_{12} & \bR_{1,3:p} \\
0 & r_2 & \bR_{2,3:p} \\
\bzero &  \bzero &\bC_3
\end{bmatrixfoot}.
$$

The same process can go on. And if $\bX\in \real^{n\times p}$, after $p$ stages, we will finally triangularize $\bX = (\bH_p \bH_{p-1}\ldots\bH_1)^{-1} \bR = \bQ\bR$. Since the $\bH_i$'s are symmetric and orthogonal (Remark~\ref{remark:householder-propes}), we have orthogonal $\bQ=(\bH_p \bH_{p-1}\ldots\bH_1)^{-1} = \bH_1 \bH_2\ldots\bH_p$.

An example of a $5\times 4$ matrix is shown as follows, where $\boxtimes$ represents a value that is not necessarily zero, and \textbf{boldface} indicates the value has just been changed:
$$
\footnotesize
\begin{aligned}
\begin{sbmatrix}{\bX}
\boxtimes & \boxtimes & \boxtimes & \boxtimes \\
\boxtimes & \boxtimes & \boxtimes & \boxtimes \\
\boxtimes & \boxtimes & \boxtimes & \boxtimes \\
\boxtimes & \boxtimes & \boxtimes & \boxtimes \\
\boxtimes & \boxtimes & \boxtimes & \boxtimes
\end{sbmatrix}
&\stackrel{\bH_1}{\rightarrow}
\begin{sbmatrix}{\bH_1\bX}
\bm{\boxtimes} & \bm{\boxtimes} & \bm{\boxtimes} & \bm{\boxtimes} \\
\bm{0} & \bm{\boxtimes} & \bm{\boxtimes} & \bm{\boxtimes} \\
\bm{0} & \bm{\boxtimes} & \bm{\boxtimes} & \bm{\boxtimes} \\
\bm{0} & \bm{\boxtimes} & \bm{\boxtimes} & \bm{\boxtimes} \\
\bm{0} & \bm{\boxtimes} & \bm{\boxtimes} & \bm{\boxtimes}
\end{sbmatrix}
\stackrel{\bH_2}{\rightarrow}
\begin{sbmatrix}{\bH_2\bH_1\bX}
\boxtimes & \boxtimes & \boxtimes & \boxtimes \\
0 & \bm{\boxtimes} & \bm{\boxtimes} & \bm{\boxtimes} \\
0 & \bm{0} & \bm{\boxtimes} & \bm{\boxtimes} \\
0 & \bm{0} & \bm{\boxtimes} & \bm{\boxtimes} \\
0 & \bm{0} & \bm{\boxtimes} & \bm{\boxtimes}
\end{sbmatrix}
\stackrel{\bH_3}{\rightarrow}
\begin{sbmatrix}{\bH_3\bH_2\bH_1\bX}
\boxtimes & \boxtimes & \boxtimes & \boxtimes \\
0 & \boxtimes & \boxtimes & \boxtimes \\
0 & 0 & \bm{\boxtimes} & \bm{\boxtimes} \\
0 & 0 & \bm{0} & \bm{\boxtimes} \\
0 & 0 & \bm{0} & \bm{\boxtimes}
\end{sbmatrix}
\stackrel{\bH_4}{\rightarrow}
\begin{sbmatrix}{\bH_4\bH_3\bH_2\bH_1\bX}
\boxtimes & \boxtimes & \boxtimes & \boxtimes \\
0 & \boxtimes & \boxtimes & \boxtimes \\
0 & 0 & \boxtimes & \boxtimes \\
0 & 0 & 0 & \bm{\boxtimes} \\
0 & 0 & 0 & \bm{0}
\end{sbmatrix}.
\end{aligned}
$$

The Householder algorithm is a procedure  that transforms a matrix into  triangular form through a sequence of orthogonal matrix operations. In the Gram-Schmidt process (both CGS and MGS), we use a triangular matrix to orthogonalize the matrix. 
In contrast, the Householder algorithm employs orthogonal matrices for matrix triangularization. 
The difference between these two approaches can be summarized as follows:
\begin{itemize}
\item Gram-Schmidt algorithm: triangular orthogonalization; 
\item Householder algorithm: orthogonal triangularization.
\end{itemize}

We further notice that, in the Householder algorithm, a set of orthogonal matrices are applied so that the QR decomposition obtained is a \textit{full} QR decomposition. Whereas, the direct QR decomposition obtained by CGS or MGS is a \textit{reduced} decomposition (although the silent columns or rows can be further added to find the full version).

\index{QR decomposition}
\begin{theoremHigh}[Algorithm complexity: QR via Householder \citep{lu2021numerical}]\label{theorem:qr-full-householder}
Householder QR algorithm requires $\sim 2np^2-\frac{2}{3}p^3$ flops to compute a full QR decomposition of an $n\times p$ matrix with linearly independent columns and $n\geq p$. Further, if $\bQ$ is needed explicitly
~\footnote{In many problems, $\bQ$ may not be needed, which is referred to as a \textit{Q-less QR decomposition}. In our case, when solving the linear system $\bX\bbeta=\by$, we may construct $\bQ^\top\by=\bR^{-\top}\bX^\top\by$ and compute $\bQ^\top\by$ directly, avoiding the need to form $\bQ$.}, 
an additional $\sim 4n^2p-2np^2$ flops are required.
\end{theoremHigh}

\begin{exercise}
Following the Householder QR algorithm, 
use the Givens rotation introduced in Definition~\ref{definition:givens-rotation-in-qr} to compute the QR decomposition.
\end{exercise}

\subsection*{\textbf{{Least Squares Problems using Householder QR}}}
\paragrapharrow{LS via Householder QR.}
When using the Householder QR algorithm to solve the least squares problem, the factor $ \bQ $ is not explicitly formed but implicitly defined as $ \bQ = \bH_1 \bH_2 \ldots \bH_p $. 
By Theorem~\ref{theorem:qr-for-ls}, $\bQ^\top\by$ and the residual vector $\be=\by-\bX\widehatbbeta$ can be obtained as:
$$
\begin{bmatrix}
\bc_1 \\
\bc_2
\end{bmatrix}
= \bH_p \ldots \bH_2 \bH_1\by,
\qquad
\be = \bH_1 \bH_2 \ldots \bH_p
\begin{bmatrix}
\bzero \\
\bc_2
\end{bmatrix}.
$$

The Householder QR factorization can be applied to the extended matrix $ [\bX , \by] $,
$$
[\bX , \by]
= \widetildebQ
\begin{bmatrix}
\bR_1 & \bc_1 \\
\bzero & \rho \be_1
\end{bmatrix},
\quad
\widetildebQ = \bH_1 \ldots \bH_p \bH_{p+1}.
$$
Then $ \bR_1 \bbeta = \bc_1 $ and the residual and its norm are given by
$$
\be = \bH_1 \ldots \bH_p \bH_{p+1}
\begin{bmatrix}
\bzero \\
\rho \be_1
\end{bmatrix},
\qquad
\normtwo{\be} = \rho \geq 0.
$$

The (Q-less) Householder QR factorization requires $\sim 2np^2 - \frac{2}{3}p^3 $ flops, and computing $ \bQ^\top \by $ and solving $ \bR_1\bbeta = \bc_1 $ require a further $\sim 4np - p^2 $ flops. If one wants not only $ \normtwo{\be} $ but also $ \be $, another $ 4np - 2p^2 $ flops are needed. This can be compared to the method of normal equation using the Cholesky decomposition, which requires $ \sim np^2 + \frac{1}{3}p^3 $ flops for the factorization and $\sim  2(np + p^2) $ flops for each right-hand side. For $ n = p $ this is about the same as for the Householder QR method, but for $ n \gg p $, the Householder method is roughly twice as expensive.

\index{Normal equation}
\index{Minimum-norm solution}
\index{QR decomposition}
\paragrapharrow{Minimum-Norm LS via Householder QR.}
Similarly, we consider the underdetermined  and consistent linear system $ \bX^\top \balpha = \bz $, where $ \bX^\top \in \real^{p\times n} $ has full row rank $ p $. Then the minimum-norm solution $ \balpha \in \real^n $ can also be computed using the Householder QR factorization of $ \bX =\bQ\begin{bmatrixfoot}
\bR_1 \\
\bzero 
\end{bmatrixfoot}$. 
Recall that the minimum-norm solution must satisfy the normal equation of the second kind in \eqref{equation:consis_minimunorm}. 
From the factorization, we have $ \bX^\top = [\bR_1^\top, \bzero] \bQ^\top $, and thus the system becomes
$$
\bX^\top \balpha = [\bR_1^\top,  \bzero] \bd = \bz, \quad \text{with } \bd \triangleq  \bQ^\top \balpha \triangleq \begin{bmatrix} \bd_1 \\ \bd_2 \end{bmatrix}.
$$
Since the $\ell_2$ norm is orthogonally invariant, i.e., $ \normtwo{\balpha} = \normtwo{\bd} $, the problem reduces to $ \min \normtwo{\bd_1} $ subject to the constraint $ \bR_1^\top \bd_1 = \bz $. 
Because $\bd_1 = \bR_1^{-\top}\bz$ is uniquely determined by this constraint,  the minimum-norm solution is obtained by setting $ \bd_2 = \bzero $, and
$$
\bR_1^\top \bd_1 = \bz, \qquad \balpha = \bQ \begin{bmatrix} \bd_1 \\ \bzero \end{bmatrix}.
$$

\paragrapharrow{Augmented LS problem via Householder QR.}
As mentioned in \eqref{equation:sys_aug_sys},  both the least squares and minimum-norm problems are special cases of the following augmented LS problem:
\begin{equation}\label{equation:sys_aug_sys_inqr}
(\text{AuLS}):\qquad 
\begin{bmatrix}
\bI & \bX \\
\bX^\top & \bzero
\end{bmatrix}
\begin{bmatrix}
\balpha \\
\bbeta
\end{bmatrix}
=
\begin{bmatrix}
\by \\
\bz
\end{bmatrix},
\quad \by \in \real^n, \quad \bz \in \real^p,
\end{equation}
where $\bX \in \real^{n\times p}$ and $\rank(\bX) = p$. From the QR factorization of $\bX$ we obtain
$$
\balpha + \bQ
\begin{bmatrix}
\bR_1 \\
\bzero
\end{bmatrix}
\bbeta = \by
\qquad \text{and} \qquad
[\bR_1^\top, \bzero]
\bQ^\top \balpha = \bz.
$$
Multiplying the first equation by $\bQ^\top$ and the second by $\bR_1^{-\top}$, we obtain:
$$
\bQ^\top \balpha + \begin{bmatrix} \bR_1 \\ \bzero \end{bmatrix} \bbeta = \bQ^\top \by
\qquad \text{and} \qquad
[\bI_p , \bzero] \bQ^\top \balpha = \bR_1^{-\top} \bz.
$$
From the second equation, we can determine the first $p$ components of $\bQ^\top \balpha$. 
These can then be substituted into the first equation to solve for $\bbeta$. 
The last $n-p$ components of $\bQ^\top \balpha$ are obtained from the last $n-p$ equations in the first equation. The resulting QR algorithm for solving the augmented system \eqref{equation:sys_aug_sys_inqr} using QR factorization is summarized below.

\begin{theoremHigh}[Augmented LS solution by  QR]
Compute the  QR factorization of $\bX \in \real^{n\times p}$, $\rank(\bX) = p$, and
\begin{align}
\bd_1 &\triangleq  \bR_1^{-\top} \bz, \qquad \bc \triangleq \begin{bmatrix} \bc_1 \\ \bc_2 \end{bmatrix} \triangleq \bQ^\top \by  \\
\implies\quad  \widehatbbeta &= \bR_1^{-1} (\bc_1 - \bd_1), \qquad \widehatbalpha = \bQ \begin{bmatrix} \bd_1 \\ \bc_2 \end{bmatrix}. 
\end{align}
\end{theoremHigh}
When $\bz=\bzero$, then $\widehatbbeta=\bR_1^{-1}\bc_1$ recovers the least squares solution presented in Theorem~\ref{theorem:qr-for-ls}; when $\bbeta=\bzero$ and $\by=\balpha$, then $\bc_1=\bzero$ obtains the minimum-norm solution for the consistent system.

The algorithm involves solving triangular systems with $\bR_1$ and $\bR_1^\top$, as well as multiplying vectors by $\bQ$ and $\bQ^\top$. 
These operations amount to a total of approximately $8np - 2p^2$ flops.

\paragrapharrow{Weighted LS problems.}
We will discuss numerical methods for solving the generalized least squares problems \eqref{equation:gls_prob_loss} in Section~\ref{section:gls_ellipmgs}, of which the WLS problem is a special case.
The WLS problem can be formulated as 
\begin{equation}\label{equation:wls_in_qr}
\min_{\bbeta} \normtwo{\bW(\bX\bbeta - \by)}^2, \quad \bW = \bOmega^{-1/2} = \diag(w_1, w_2, \ldots, w_n),
\end{equation}
where weights $ w_i $ can be understood such that the weighted residuals $ e_i = w_i(\by - \bX\bbeta)_i $ have equal variance; see Remark~\ref{remark:samp_gls}. 
Note that the solution to \eqref{equation:wls_in_qr} is scale-invariant, i.e., it does not change if $ \bW $ is multiplied by a nonzero scalar. Therefore, without loss of generality, we can assume in the following that $ w_i \geq 1 $, and that the rows of $ \bX $ are normalized so that $ \max_{1 \leq j \leq p} \abs{x_{ij}} = 1 $, $ i = 1, 2,\ldots, n $.
The solution to the WLS problem \eqref{equation:wls_in_qr} satisfies the normal equation; see \eqref{equation:gls_prob_gne}:
\begin{equation}
\bX^\top \bW^2 \bX \bbeta = \bX^\top \bW^2 \by.
\end{equation}
A more numerically stable solution method involves using the \textit{weighted QR factorization} $ \bW \bX = \bQ\bR $. The solution to \eqref{equation:wls_in_qr} is then obtained by solving
\begin{equation}
\bR\bbeta = \bQ^\top \bW \by,
\end{equation}
thus avoiding the need to explicitly square the weight matrix $ \bW $.

For a consistent underdetermined system $ \bX^\top \balpha = \bz $, the unique solution of the \textit{weighted minimum-norm problem}
\begin{equation}
\min_{\balpha} \normtwo{\bW \balpha}  \quad \text{s.t. } \bX^\top \balpha = \bz
\end{equation}
is given by the generalized normal equation of the second kind; see \eqref{equation:gls_prob_gne2}:
\begin{equation}
(\bX^\top \bW^2 \bX) \bgamma = \bz, \qquad \balpha = \bW^2 \bX \bgamma.
\end{equation}
Again, using the weighted QR factorization $ \bW \bX = \bQ\bR $ leads to a more accurate solution by avoiding squaring the weight matrix:
\begin{equation}
\balpha = \bW \bQ \bR^{-\top} \bz.
\end{equation}

\subsection{Modifying LS: Appending or Deleting a Covariate/Column}\label{section:append-column-qr}
In Section~\ref{section:cholesky-rank-one-update}, we discussed how to use the Cholesky decomposition to add or delete a data in the least squares problem efficiently. 
In certain applications, such as the $F$-test for least squares via QR decomposition (see Section~\ref{section:variable-selection}), there arises a need to remove or add a column/covariate to the observed matrix. The objective is to efficiently obtain the QR decomposition of the modified matrix.

\paragrapharrow{Deleting a column.}
Suppose the QR decomposition of $\bX\in \real^{n\times p}$ is given by $\bX=\bQ\bR$, where the column partition of $\bX$ is $\bX=[\bx_1,\bx_2,\ldots,\bx_p]$. Now, if we delete the $k$-th column of $\bX$, resulting in $\bX^\prime \triangleq [\bx_1,\ldots,\bx_{k-1},\bx_{k+1},\ldots,\bx_p] \in \real^{n\times (p-1)}$, we want to compute the QR decomposition of $\bX^\prime$ efficiently.
Additionally, let $\bR$ have the following structure:
$$
\begin{aligned}
\begin{blockarray}{ccccc}
\begin{block}{c[ccc]c}
&	\bR_{11} & \bx & \bR_{12} & k-1  \\
\bR \triangleq	&  \bzero    &	r_{kk} & \bb^\top  & 1 \\
&	\bzero  &\bzero& \bR_{22} & n-k  \\
\end{block}
& k-1 & 1 &  p-k & \\
\end{blockarray}, 
\end{aligned}
\qquad\text{such that}\qquad
\bQ^\top \bX^\prime = 
\begin{bmatrix}
\bR_{11} &\bR_{12} \\
\bzero & \bb^\top \\
\bzero & \bR_{22}
\end{bmatrix} \triangleq \bH
$$
is  \textit{upper Hessenberg}. 
An example is provided below to illustrate the case of a $6\times 5$ matrix, where $k=3$ and the $k$-th column is deleted:
$$
\begin{aligned}
\begin{sbmatrix}{\bR = \bQ^\top\bX}
\boxtimes & \boxtimes & \boxtimes & \boxtimes & \boxtimes \\
0 & \boxtimes & \boxtimes & \boxtimes & \boxtimes \\
0 & 0 & \boxtimes & \boxtimes& \boxtimes \\
0 & 0 & 0 & \boxtimes& \boxtimes \\
0 & 0 & 0 & 0& \boxtimes \\
0 & 0 & 0 & 0& 0
\end{sbmatrix}
&\quad\implies\quad
\begin{sbmatrix}{\bH = \bQ^\top\bX^\prime}
\boxtimes & \boxtimes  & \boxtimes & \boxtimes \\
0 & \boxtimes  & \boxtimes & \boxtimes \\
0 & 0 & \boxtimes& \boxtimes \\
0 & 0  & \boxtimes& \boxtimes \\
0 & 0  & 0& \boxtimes \\
0 & 0  & 0& 0
\end{sbmatrix}.
\end{aligned}
$$

Again, for columns $k$ to $p-1$ of $\bH$, there exists a set of Givens rotations (Definition~\ref{definition:givens-rotation-in-qr}) $\bG_{k,k+1}$, $\bG_{k+1,k+2}$, $\ldots$, $\bG_{p-1,p}$ that can be applied to introduce zeros in the subdiagonal elements $h_{k+1,k}$, $h_{k+2,k+1}$, $\ldots$, $h_{p,p-1}$ of $\bH$. Then the triangular matrix $\bR^\prime$ is given by 
$$
\bR^\prime \triangleq \bG_{p-1,p}\ldots \bG_{k+1,k+2}\bG_{k,k+1}\bQ^\top \bX^\prime.
$$
The orthogonal matrix can be obtained through the following procedure:
\begin{equation}\label{equation:qr-delete-column-finalq}
\bQ^\prime = (\bG_{p-1,p}\ldots \bG_{k+1,k+2}\bG_{k,k+1}\bQ^\top )^\top = \bQ \bG_{k,k+1}^\top  \bG_{k+1,k+2}^\top \ldots \bG_{p-1,p}^\top,
\end{equation}
such that $\bX^\prime = \bQ^\prime\bR^\prime$. 
The procedure is outlined in Algorithm~\ref{alg:qr-delete-a-column}. And the $6\times 5$ example is shown as follows, where $\boxtimes$ represents a value that is not necessarily zero, and \textbf{boldface} indicates the value has just been changed:
$$
\begin{aligned}
\begin{sbmatrix}{\bR = \bQ^\top\bX}
\boxtimes & \boxtimes & \boxtimes & \boxtimes & \boxtimes \\
0 & \boxtimes & \boxtimes & \boxtimes & \boxtimes \\
0 & 0 & \boxtimes & \boxtimes& \boxtimes \\
0 & 0 & 0 & \boxtimes& \boxtimes \\
0 & 0 & 0 & 0& \boxtimes \\
0 & 0 & 0 & 0& 0
\end{sbmatrix}
&\stackrel{k=3}{\rightarrow}
\begin{sbmatrix}{\bH = \bQ^\top\bX^\prime}
\boxtimes & \boxtimes  & \boxtimes & \boxtimes \\
0 & \boxtimes  & \boxtimes & \boxtimes \\
0 & 0 & \boxtimes& \boxtimes \\
0 & 0  & \boxtimes& \boxtimes \\
0 & 0  & 0& \boxtimes \\
0 & 0  & 0& 0
\end{sbmatrix}
\stackrel{\bG_{34}}{\rightarrow}
\begin{sbmatrix}{\bG_{34}\bH }
\boxtimes & \boxtimes  & \boxtimes & \boxtimes \\
0 & \boxtimes  & \boxtimes & \boxtimes \\
0 & 0 & \bm{\boxtimes}& \bm{\boxtimes} \\
0 & 0  & \bm{0}& \bm{\boxtimes} \\
0 & 0  & 0& \boxtimes \\
0 & 0  & 0& 0
\end{sbmatrix}
\stackrel{\bG_{45}}{\rightarrow}
\begin{sbmatrix}{\bG_{45}\bG_{34}\bH}
\boxtimes & \boxtimes  & \boxtimes & \boxtimes \\
0 & \boxtimes  & \boxtimes & \boxtimes \\
0 & 0 & \boxtimes & \boxtimes \\
0 & 0  &0 & \bm{\boxtimes} \\
0 & 0  & 0& \bm{0} \\
0 & 0  & 0& 0
\end{sbmatrix}.
\end{aligned}
$$

\begin{algorithm}[h] 
\caption{QR Deleting a Column} 
\label{alg:qr-delete-a-column} 
\begin{algorithmic}[1] 
\Require Matrix $\bX \in \real^{n\times p}$ with full QR decomposition $\bX=\bQ\bR$, and $\bX^\prime \in \real^{n\times (p-1)}$ by deleting column $k$ of $\bX$; 
\Statex \textbf{Stage A: Triangularize $\bH$}
\State Obtain $\bH$ by deleting column $k$ of $\bR$, that is, $\bH=\bQ^\top\bX^\prime$;
\For{$i=k$ to $p-1$} 
\State Get Givens rotation $\bG_{i,i+1}$ with the following parameters $c, s$:
\State $c \leftarrow \frac{x_k}{\sqrt{x_k^2 + x_l^2}}$, $s\leftarrow\frac{x_l}{\sqrt{x_k^2 + x_l^2}}$, where $x_k = h_{ii}$, $x_l = h_{i+1,i}$;
\State Calculate $\bH \leftarrow \bG_{i,i+1}\bH $ in following two steps:
\State $i$-th row: $\bH_{i,:} \leftarrow c\cdot \bH_{i,:} + s \bH_{j,:} $, where $j=i+1$;
\State $(i+1)$-th row: $\bH_{i+1,:} \leftarrow -s\cdot \bH_{i,:} + c \bH_{j,:} $, where $j=i+1$;
\EndFor
\State Set $\bR^\prime \leftarrow \bH$ and output $\bR^\prime$;
\Statex \textbf{Stage B: Obtain the orthogonal matrix $\bQ^\prime$}
\State Set $\bQ^\prime \leftarrow \bQ^\top $;
\For{$i=k$ to $p-1$} 
\State $c \leftarrow \frac{x_k}{\sqrt{x_k^2 + x_l^2}}$, $s\leftarrow\frac{x_l}{\sqrt{x_k^2 + x_l^2}}$, where $x_k$, $x_l$ are from step 4;
\State Calculate $\bQ^\prime \leftarrow \bG_{i,i+1}\bQ^\prime $ in following two steps:
\State $i$-th row: $\bQ^\prime_{i,:} \leftarrow c\cdot \bQ^\prime_{i,:} + s \bQ^\prime_{j,:} $, where $j=i+1$;
\State $(i+1)$-th row: $\bQ^\prime_{i+1,:} \leftarrow -s\cdot \bQ^\prime_{i,:} + c \bQ^\prime_{j,:} $, where $j=i+1$;
\EndFor
\State Output $\bQ^\prime \leftarrow \bQ^{\prime\top}$ from Equation~\eqref{equation:qr-delete-column-finalq}; 
\end{algorithmic} 
\end{algorithm}

\begin{theoremHigh}[Algorithm Complexity: QR Deleting Column \citep{lu2021numerical}]\label{theorem:qr-full-givens-delete-column}
Algorithm~\ref{alg:qr-delete-a-column}  requires $\sim 3p^2-6pk+3k^2$ flops to compute a full QR decomposition of the matrix $\bX^\prime \in \real^{n\times (p-1)}$.
This matrix is obtained by deleting  $k$-th column from $\bX\in \real^{n\times p}$, assuming that the full QR decomposition of $\bX$ is already known. 
Furthermore, if the orthogonal matrix $\bQ^\prime$ needs to be formed explicitly, an additional $\sim 6n(p-k)$ flops are required.
\end{theoremHigh}

Note that the value of column $k$ affects the complexity: when $k=p$, the complexity is $\sim 0$ (ignoring constant terms); and when $k=1$, the complexity reaches its maximum value.

\paragrapharrow{Appending a column.}
Similarly, suppose $\widetilde{\bX} \triangleq [\bx_1,\bx_k,\bw,\bx_{k+1},\ldots,\bx_p]$, where we append a vector $\bw$ to the $(k+1)$-th column of $\bX$. Applying the orthogonal transformation $\bQ^\top$, we have
$$
\bQ^\top \widetilde{\bX} = [\bQ^\top\bx_1,\ldots, \bQ^\top\bx_k, \bQ^\top\bw, \bQ^\top\bx_{k+1}, \ldots,\bQ^\top\bx_p] \triangleq \widetilde{\bH}.
$$
A set of Givens rotations $\bJ_{n-1,n}, \bJ_{n-2,n-1}, \ldots, \bJ_{k+1,k+2}$ can introduce zeros in the elements $\widetilde{h}_{n,k+1}$, $\widetilde{h}_{n-1,k+1}$, $\ldots$, $\widetilde{h}_{k+2,k+1}$ of $\widetilde{\bH}$, thereby achieving the desired result of the updated QR decomposition. That is,
$$
\widetilde{\bR} \triangleq\bJ_{k+1,k+2}\ldots \bJ_{n-2,n-1}\bJ_{n-1,n} \bQ^\top \widetilde{\bX},
$$
is upper triangular.
Suppose $\widetilde{\bH}$ is a $6\times 5$ matrix and $k=2$. An example is shown below, where $\boxtimes$ represents a value that is not necessarily zero, and \textbf{boldface} indicates the value has just been changed:
$$
\footnotesize
\begin{aligned}
\begin{sbmatrix}{\widetilde{\bH}}
\boxtimes & \boxtimes & \boxtimes & \boxtimes & \boxtimes \\
0 & \boxtimes & \boxtimes & \boxtimes & \boxtimes \\
0 & 0 & \boxtimes & \boxtimes& \boxtimes \\
0 & 0 & \boxtimes & 0& \boxtimes \\
0 & 0 & \boxtimes & 0& 0 \\
0 & 0 & \boxtimes & 0& 0
\end{sbmatrix}
&\stackrel{\bJ_{56}}{\rightarrow}
\begin{sbmatrix}{\bJ_{56}\widetilde{\bH} \rightarrow \widetilde{h}_{63}=0}
\boxtimes & \boxtimes & \boxtimes & \boxtimes & \boxtimes \\
0 & \boxtimes & \boxtimes & \boxtimes & \boxtimes \\
0 & 0 & \boxtimes & \boxtimes& \boxtimes \\
0 & 0 & \boxtimes & 0& \boxtimes \\
0 & 0 & \bm{\boxtimes} & 0& 0 \\
0 & 0 & \bm{0} & 0& 0
\end{sbmatrix}
\stackrel{\bJ_{45}}{\rightarrow}
\begin{sbmatrix}{\bJ_{45}\bJ_{56}\widetilde{\bH}\rightarrow \widetilde{h}_{53}=0}
\boxtimes & \boxtimes & \boxtimes & \boxtimes & \boxtimes \\
0 & \boxtimes & \boxtimes & \boxtimes & \boxtimes \\
0 & 0 & \boxtimes & \boxtimes& \boxtimes \\
0 & 0 & \bm{\boxtimes} & 0& \bm{\boxtimes} \\
0 & 0 & \bm{0} & 0& \bm{\boxtimes}\\
0 & 0 & 0 & 0& 0
\end{sbmatrix}
\stackrel{\bJ_{34}}{\rightarrow}
\begin{sbmatrix}{\bJ_{34}\bJ_{45}\bJ_{56}\widetilde{\bH}\rightarrow \widetilde{h}_{43}=0}
\boxtimes & \boxtimes & \boxtimes & \boxtimes & \boxtimes \\
0 & \boxtimes & \boxtimes & \boxtimes & \boxtimes \\
0 & 0 & \bm{\boxtimes}  & \bm{\boxtimes} & \bm{\boxtimes}  \\
0 & 0 & \bm{0}  & \bm{\boxtimes} & \bm{\boxtimes}  \\
0 & 0 & 0 & 0& \boxtimes\\
0 & 0 & 0 & 0& 0
\end{sbmatrix} = \widetilde{\bR}.
\end{aligned}
$$ 
Finally, we obtain the orthogonal matrix 
\begin{equation}\label{equation:qr-add-column-finalq}
\widetilde{\bQ} = (\bJ_{k+1,k+2}\ldots \bJ_{n-2,n-1}\bJ_{n-1,n} \bQ^\top )^\top = \bQ \bJ_{n-1,n}^\top  \bJ_{n-2,n-1}^\top \ldots \bJ_{k+1,k+2}^\top,
\end{equation}
such that $\widetilde{\bX} = \widetilde{\bQ}\widetilde{\bR}$. The procedure is formulated in Algorithm~\ref{alg:qr-adding-a-column}.

\begin{algorithm}[h] 
\caption{QR Adding a Column} 
\label{alg:qr-adding-a-column} 
\begin{algorithmic}[1] 
\Require Matrix $\bX \in \real^{n\times p}$ with full QR decomposition $\bX=\bQ\bR$, and $\widetilde{\bX}\in \real^{n\times (p+1)}$ by adding column $\bw$ into $(k+1)$-th column of $\bX$; 
\Statex \textbf{Stage A: Triangularize $\widetilde{\bH}$}
\State Calculate $\bQ^\top\bw$;
\State Obtain $\widetilde{\bH}$ by inserting $\bQ^\top \bw$ into $(k+1)$-th column of $\bR$;
\For{$i=n-1$ to $k+1$} 
\State Get Givens rotation $\bJ_{i,i+1}$ with the following parameters $c, s$:
\State $c \leftarrow \frac{x_k}{\sqrt{x_k^2 + x_l^2}}$, $s\leftarrow\frac{x_l}{\sqrt{x_k^2 + x_l^2}}$, where $x_k = \widetilde{h}_{i,k+1}$, $x_l = \widetilde{h}_{i+1,k+1}$;
\State Calculate $\widetilde{\bH} \leftarrow \bJ_{i,i+1}\bH $ in following two steps:
\State $i$-th row: $\widetilde{\bH}_{i,:} \leftarrow c\cdot \widetilde{\bH}_{i,:} + s \widetilde{\bH}_{j,:} $, where $j=i+1$;
\State $(i+1)$-th row: $\widetilde{\bH}_{i+1,:} \leftarrow -s\cdot \widetilde{\bH}_{i,:} + c \widetilde{\bH}_{j,:} $, where $j=i+1$;
\EndFor
\State Set $\widetilde{\bR} \leftarrow \widetilde{\bH}$ and output $\widetilde{\bR}$;
\Statex \textbf{Stage B: Obtain the orthogonal matrix $\widetilde{\bQ}$}
\State Set $\widetilde{\bQ} \leftarrow \bQ^\top $;
\For{$i=n-1$ to $k+1$} 
\State $c \leftarrow \frac{x_k}{\sqrt{x_k^2 + x_l^2}}$, $s\leftarrow\frac{x_l}{\sqrt{x_k^2 + x_l^2}}$, where $x_k$, $x_l$ are from step 5;
\State Calculate $\widetilde{\bQ} \leftarrow \bJ_{i,i+1}\widetilde{\bQ} $ in following two steps:
\State $i$-th row: $\widetilde{\bQ}_{i,:} \leftarrow c\cdot \widetilde{\bQ}_{i,:} + s \widetilde{\bQ}_{j,:} $, where $j=i+1$;
\State $(i+1)$-th row: $\bQ_{i+1,:} \leftarrow -s\cdot \widetilde{\bQ}_{i,:} + c \widetilde{\bQ}_{j,:} $, where $j=i+1$;
\EndFor
\State Output $\widetilde{\bQ} \leftarrow \widetilde{\bQ}^{\top}$ from Equation~\eqref{equation:qr-add-column-finalq}; 
\end{algorithmic} 
\end{algorithm}
\begin{theoremHigh}[Algorithm Complexity: QR Adding Column \citep{lu2021numerical}]\label{theorem:qr-full-givens-add-column}
Algorithm~\ref{alg:qr-adding-a-column} requires $\sim 2n^2+6(np+k^2-pk-nk)$ flops to compute a full QR decomposition of the matrix $\widetilde{\bX} \in \real^{n\times (p+1)}$, where we add a column to  the $(k+1)$-th column of $\bX\in \real^{n\times p}$ and the full QR decomposition of $\bX$ is known. Furthermore, if the orthogonal matrix $\widetilde{\bQ}$ needs to be formed explicitly, an additional $\sim 6n(n-k)$ flops are required.
\end{theoremHigh}

Note that the column number $k$ plays a significant role in determining the complexity. 
When $k=p$, the complexity is $2n^2$ flops. On the other hand, when $k=1$, the complexity reaches its maximum value.

\paragrapharrow{Other Real-world application.} 
The method described above is useful for performing efficient variable selection in least squares problems using QR decomposition. In this approach, we iteratively remove a column from the data matrix $\bX$ and perform an $F$-test to determine the significance of the corresponding variable. If the variable is deemed insignificant, it is removed, leading to a simpler model. A brief overview is given below; for more details, refer to Section~\ref{section:variable-selection}. 

Following the setup described in Section~\ref{section:application-ls-qr}, 
let's consider the overdetermined system $\bX\bbeta = \by$, where $\bX\in \real^{n\times p}$ represents the data matrix, and $\by\in \real^n$ with $n\geq p$ is the observed response. 
The LS solution is obtained by minimizing $\normtwo{\bX\bbeta-\by}^2$, and it can be expressed as $\bbeta_{LS} = (\bX^\top\bX)^{-1}\bX^\top\by$.

Suppose we remove a column from $\bX$ to obtain $\widehat{\bX}$. Consequently, the LS solution changes from $\bbeta_{LS}$ to $\widehat{\bbeta}_{LS}$.
Define 
$$
\begin{aligned}
RSS(\widehat{\bbeta}_{LS}) &\triangleq \normtwo{\by - \widehat{\by}_{LS}}^2, \qquad \text{where~ } \widehat{\by}_{LS} = \widehat{\bX}\widehat{\bbeta}_{LS},   \\
RSS(\bbeta_{LS})&\triangleq \normtwo{\by - \by_{LS}}^2, \qquad \text{where~ } \by_{LS} = \bX\bbeta_{LS},\\
\bH & \triangleq \bX(\bX^\top\bX)^{-1}\bX^\top, \\
\widehat{\bH} &\triangleq \widehat{\bX}(\widehat{\bX}^\top \widehat{\bX})^{-1}\widehat{\bX}^\top. \\
\end{aligned}
$$
Suppose the \textit{reduced} QR decompositions of $\bX$ and $\widehat{\bX}$ are given by $\bX=\bQ\bR$ and $\widehat{\bX}=\widehat{\bQ}\widehat{\bR}$.
Thus $RSS(\bbeta_{LS}) = \by^\top (\bI-\bH)\by = \by^\top\by - (\by^\top\bQ)(\bQ^\top\by)$
and $RSS(\widehat{\bbeta}_{LS}) - RSS(\bbeta_{LS}) = \normtwo{\by_{LS} - \widehat{\by}_{LS}}^2 = \by^\top(\bH-\widehat{\bH})\by=(\by^\top\bQ)(\bQ^\top\by)-(\by^\top\widehat{\bQ})(\widehat{\bQ}^\top\by)$, which are the differences of two inner products.
It can be shown that $RSS(\bbeta_{LS})\sim \sigma^2 \chi^2_{(n-p)}$, which follows a Chi-square distribution, and $\sigma$ is the noise level. Under the hypothesis that the deleted column is not significant, we could conclude that 
$$
T={\frac{1}{p-q}\big(RSS(\widehat{\bbeta}_{LS}) - RSS(\bbeta_{LS})\big)  }/{\frac{1}{n-p}RSS(\bbeta_{LS})} \sim F_{p-q,n-p},
$$
which is the \textbf{test statistic for the $F$-test} with $q=p-1$.
Suppose we have the data set $(\bx_1, y_1)$, $(\bx_2, y_2)$, $\ldots$, $(\bx_n, y_n)$, and we observe $T=t$ for this particular data set. Then 
$$
\widetildep=P[T((\bx_1, y_1), (\bx_2, y_2), \ldots, (\bx_n, y_n)) \geq t] = P[F_{p-q,n-p} \geq t].
$$
We reject the hypothesis (i.e., the  variable $k$ should be removed) if $\widetildep<\alpha$, for some small $\alpha$, say 0.05. This is known as the \textit{p-value}.

\subsection{Modifying LS: Appending or Deleting a Data/Row}\label{section:qr_adddelrow}
Similarly, we may also need to append or delete a row/data from the observed matrix in the QR case. 
In such cases, our objective is to efficiently compute the QR decomposition of the modified matrix.

\paragrapharrow{Appending a row.}
Suppose the full QR decomposition of $\bX\in \real^{n\times p}$ is given by $\bX= \scriptsize\begin{bmatrix}
\bX_1 \\
\bX_2 
\end{bmatrix}=\bQ\bR$, where $\bX_1\in \real^{k\times p}$ and $\bX_2 \in \real^{(n-k)\times p}$. Now, if we append a row such that $\bX^\prime = \scriptsize\begin{bmatrix}
\bX_1 \\
\bw^\top\\
\bX_2 
\end{bmatrix} \in \real^{(n+1)\times p}$, we aim to efficiently obtain the full QR decomposition of $\bX^\prime$. 
To achieve this, we construct a permutation matrix, denoted by $\bP$, such that
$$
\bP=
\begin{bmatrix}
\bzero & 1 & \bzero  \\
\bI_k & \bzero & \bzero \\
\bzero & \bzero & \bI_{n-k}
\end{bmatrix}
\longrightarrow
\bP 
\begin{bmatrix}
\bX_1 \\
\bw^\top\\
\bX_2 
\end{bmatrix}
=
\begin{bmatrix}
\bw^\top\\
\bX_1 \\
\bX_2 
\end{bmatrix}
\quad 
\implies
\quad
\begin{bmatrix}
1 & \bzero \\
\bzero & \bQ^\top 
\end{bmatrix} 
\bP 
\bX^\prime 
=
\begin{bmatrix}
\bw^\top \\
\bR 
\end{bmatrix}=\bH 
$$
is upper Hessenberg. Similarly, a set of rotations $\bG_{12}, \bG_{23}, \ldots, \bG_{p,p+1}$ can be applied to introduce zeros in the elements $h_{21}$, $h_{32}$, $\ldots$, $h_{p+1,p}$ of $\bH$. The resulting triangular matrix $\bR^\prime$ is given by 
$$
\bR^\prime \triangleq \bG_{p,p+1}\ldots \bG_{23}\bG_{12}\begin{bmatrix}
1 & \bzero \\
\bzero & \bQ^\top 
\end{bmatrix} 
\bP   \bX^\prime.
$$
And the orthogonal matrix  is obtained by
$$
\bQ^\prime \triangleq \left(\bG_{p,p+1}\ldots \bG_{23}\bG_{12}\begin{bmatrix}
1 & \bzero \\
\bzero & \bQ^\top 
\end{bmatrix} 
\bP \right)^\top 
= 
\bP^\top 
\begin{bmatrix}
1 & \bzero \\
\bzero & \bQ 
\end{bmatrix}  
\bG_{12}^\top  \bG_{23}^\top \ldots \bG_{p,p+1}^\top.
$$
Thus, the QR decomposition of $\bX^\prime$ is  $\bX^\prime = \bQ^\prime\bR^\prime$.

\paragraph{Update of least squares problem.}
In the context of least squares, each row of $\bX$ and $\by$ is  referred to as an \textit{observation}.
In real-world application, new observation may be received. 
When performing the optimization process from scratch, obtaining the solution of the least squares problem would require approximately $\sim 2(n+1)p^2-\frac{2}{3}p^3$ flops. 
Let's consider a new observation $[\bx^\top,b]$, leading to the following reduction:
$$
\left[
\begin{array}{c|c}
%\cline{1-1}
\bX & \by  \\
\bx^\top & b
\end{array}
\right]
\rightarrow 
\begin{bmatrix}
\bQ^\top & \bzero \\
\bzero & 1
\end{bmatrix}
\left[
\begin{array}{c|c}
\bX & \by  \\
\bx^\top & b
\end{array}
\right]
=
\underbrace{\begin{bmatrix}
\bR_1 & \bQ_1^\top\by \\
\bzero & \bQ_2^\top\by \\
\bx^\top & b
\end{bmatrix}}_{\triangleq \bZ}
$$
Therefore, the updated least squares solution is obtained by transforming $\bZ\in\real^{(n+1)\times (p+1)}$ into an upper triangular matrix (actually, we transform the left $p$ columns of $\bZ$ into an upper triangular matrix). This can be done by a set of rotations in the $(1,n+1)$ plane, $(2,n+1)$ plane, $\ldots$, $(p,n+1)$ plane that introduce zero to $(n+1, 1), (n+1,2),\ldots,(n+1,p)$-th entry of $\bZ$, respectively. 
The computational cost for this operation is $\mathcalO(np)$ flops.

\paragrapharrow{Deleting a row.} 
Conversely, suppose $\bX = 
\scriptsize
\begin{bmatrix}
\bX_1 \\
\bw^\top\\
\bX_2 
\end{bmatrix} \in \real^{n\times p}$, where $\bX_1\in\real^{k\times p}$ and $\bX_2 \in \real^{(n-k-1)\times p}$. The full QR decomposition of $\bX$ is given by $\bX=\bQ\bR$, where $\bQ\in \real^{n\times n}$ and $\bR\in \real^{n\times p}$. We want to compute the full QR decomposition of 
$\widetilde{\bX} =  
\footnotesize
\begin{bmatrix}
\bX_1 \\
\bX_2 
\end{bmatrix}$ efficiently (assuming $n-1\geq p$). Analogously, we can construct a permutation matrix 
$$
\bP = 
\begin{bmatrix}
\bzero & 1 & \bzero  \\
\bI_k & \bzero & \bzero \\
\bzero & \bzero & \bI_{n-k-1}
\end{bmatrix}
\implies 
\bP\bX = 
\begin{bmatrix}
\bzero & 1 & \bzero  \\
\bI_k & \bzero & \bzero \\
\bzero & \bzero & \bI_{n-k-1}
\end{bmatrix}
\begin{bmatrix}
\bX_1 \\
\bw^\top\\
\bX_2 
\end{bmatrix}=
\begin{bmatrix}
\bw^\top \\
\bX_1\\
\bX_2 
\end{bmatrix} = \bP\bQ\bR \triangleq\bM\bR ,
$$
where $\bM \triangleq \bP\bQ$ is an orthogonal matrix. 
Let $\bmm^\top$ be the first row of $\bM$, and construct a set of Givens rotations $\bG_{n-1,n}, \bG_{n-2,n-1}, \ldots, \bG_{1,2}$, which introduce zeros in the elements $m_n, m_{n-1}, \ldots, m_2$ of $\bmm$, respectively. 
By applying these rotations, we can obtain $\bG_{1,2}\ldots \bG_{n-2,n-1}\bG_{n-1,n}\bmm = \alpha \be_1$, where $\alpha = \pm 1$. Therefore, we obtain the following result:
$$
\bG_{1,2}\ldots \bG_{n-2,n-1}\bG_{n-1,n} \bR \triangleq
\begin{blockarray}{cc}
\begin{block}{[c]c}
\bv^\top & 1 \\
\bR_1& n-1  \\
\end{block}
\end{blockarray},
$$ 
which is upper Hessenberg with $\bR_1\in \real^{(n-1)\times p}$ being upper triangular. And 
$$
\bM   \bG_{n-1,n}^\top \bG_{n-2,n-1}^\top \ldots  \bG_{1,2}^\top  \triangleq
\begin{bmatrix}
\alpha & \bzero \\
\bzero & \bQ_1 
\end{bmatrix},
$$
where $\bQ_1\in \real^{(n-1)\times (n-1)}$ is an orthogonal matrix. The bottom-left block of the  matrix above is a zero vector since $\alpha=\pm 1$ and $\bM$ is orthogonal. To see this, let $\bG\triangleq\bG_{n-1,n}^\top \bG_{n-2,n-1}^\top $ $\ldots  \bG_{1,2}^\top $, with its first  column denoted as  $\bg$. Writing $\bM$ as the row partition $\bM = [\bmm^\top; \bmm_2^\top; \bmm_3^\top; \ldots, \bmm_{n}^\top]$,  we have
$$
\begin{aligned}
\bmm^\top\bg &= \pm 1 \qquad  \rightarrow \qquad  \bg = \pm \bmm, \\
\bmm_i^\top \bmm &=0,  \qquad \forall\, i \in \{2,3,\ldots,n\}.
\end{aligned}
$$
This results in 
$$
\begin{aligned}
\bP\bX&=\bM\bR
=(\bM   \bG_{n-1,n}^\top \bG_{n-2,n-1}^\top \ldots  \bG_{1,2}\top ) (\bG_{1,2}\ldots \bG_{n-2,n-1}\bG_{n-1,n} \bR ) \\
&=
\begin{bmatrix}
\alpha & \bzero \\
\bzero & \bQ_1 
\end{bmatrix}
\begin{bmatrix}
\bv^\top \\
\bR_1 
\end{bmatrix} = 
\begin{bmatrix}
\alpha \bv^\top \\
\bQ_1\bR_1 
\end{bmatrix}
=
\begin{bmatrix}
\bw^\top \\
\widetilde{\bX}
\end{bmatrix}
.
\end{aligned}
$$
Thus, we obtain that $\bQ_1\bR_1$ is the full QR decomposition of $\widetilde{\bX}=
\begin{bmatrixfoot}
\bX_1\\
\bX_2 
\end{bmatrixfoot}$.

%\textcolor{red}{TODO, todelete}
%
%\paragraph{Complexity.} 
%In the least squares solution using QR decomposition, it is unnecessary to compute the orthogonal matrix $\bQ$ entirely. This can be observed by performing the \textit{Householder algorithm} \footnote{which contains two parts: the first one requires $\sim 2np^2-\frac{2}{3}p^3$ flops to obtain the upper triangular matrix $\bR$, and the second one requires $\sim 4n^2p-2np^2$ flops to get the orthogonal matrix $\bQ$ in the full QR decomposition; see \citet{lu2021numerical}.} on the augmented matrix $[\bX\mid\by]$:
%\begin{equation}\label{equation:ls_hourse}
%	[\bX~|~\by] \rightarrow \bQ^\top[\bX~|~\by] 
%	=
%	\begin{bmatrix}
%		\bQ_1^\top \\
%		\bQ_2^\top
%	\end{bmatrix}
%	[\bX~|~\by]
%	=
%	\begin{bmatrix}
%		\bR_1 & \bQ_1^\top\by \\
%		\bzero & \bQ_2^\top\by
%	\end{bmatrix},
%\end{equation}
%where $\bQ=[\bQ_1, \bQ_2]\in\real^{n\times n}$ with $\bQ_1\in\real^{n\times p}$ and $\bQ_2 \in\real^{n\times (n-p)}$. 
%Applying the Householder algorithm allows us to directly compute $\bR_1$ and $\bQ_1^\top\by$ without explicitly having $\bQ$ or $\bQ_1$.
%Hence, the complexity of the least squares method using QR decomposition corresponds to the first part of the Householder algorithm, which is approximately $\sim 2np^2-\frac{2}{3}p^3$ flops. 
%The computation cost of obtaining $\bQ_1^\top\by$ is not included in the final flops count, as it requires $\mathcalO(np)$ flops (not included in the leading term).
%
%
%\textcolor{red}{TODO, todelete}

\index{Rank-one update}
\subsection{Modifying LS: Rank-One Changes}\label{section:qr-rank-one-changes}
We have previously explored the rank-one update and downdate of the Cholesky decomposition in Section~\ref{section:cholesky-rank-one-update}. 
The rank-one change of the matrix $\bX$ in the QR decomposition, denoted by $\bX^\prime$, is defined in a similar manner:
$$
\begin{aligned}
\bX^\prime &= \bX + \bu\bv^\top, \\
\downarrow &\gap  \downarrow\\
\bQ^\prime\bR^\prime &=\bQ\bR + \bu\bv^\top,
\end{aligned}
$$
where if we set $\bX^\prime = \bX - (-\bu)\bv^\top$, we recover the downdate form such that the update and downdate in the QR decomposition are equivalent. 

To reiterate,
the rank-one update/downdate problem involves finding the QR decomposition of $\bX^\prime$ when the QR decomposition of $\bX$ has already been computed.
Let $\bw \triangleq \bQ^\top\bu$, then we have
$
\bX^\prime = \bQ(\bR + \bw\bv^\top).
$
Similarly, 
there exists a set of Givens rotations $\bG_{12} \bG_{23} \ldots \bG_{(p-1),p}$ that satisfy:
$$
\bG_{12} \bG_{23} \ldots \bG_{(p-1),p} \bw = \pm \normtwo{\bw} \be_1,
$$
where $\bG_{(k-1),k}$ represents the Givens rotation effecting in the $(k-1)$-th and $k$-th plane, which introduces zero in the $k$-th entry of $\bw$. 

Applying these rotations to $\bR$, we obtain
$
\bG_{12} \bG_{23} \ldots \bG_{(p-1),p}\bR \triangleq \bH_0 ,
$
where the Givens rotations in this \textit{reverse order} (\textit{backward rotations}) are employed to transform the upper triangular matrix $\bR$ into a ``simple" \textit{upper Hessenberg matrix}.
However, if the rotations transform $\bw$ into $\pm \normtwo{\bw}\be_1$ in the \textit{forward order} (\textit{forward rotations}), we will not obtain this upper Hessenberg $\bH_0$. 

To see this, consider a $5\times 5$ matrix $\bR$, an example is shown as follows, where $\boxtimes$ represents a value that is not necessarily zero, and \textbf{boldface} indicates the value has just been changed. Using backward rotations results in the upper Hessenberg $\bH_0$, which is easier to manage to update the QR decomposition:
$$
\begin{aligned}
\text{\parbox{7em}{Backward\\(Right Way)}: }
\footnotesize
\begin{sbmatrix}{\bR}
\boxtimes & \boxtimes & \boxtimes & \boxtimes& \boxtimes \\
0 & \boxtimes & \boxtimes & \boxtimes& \boxtimes \\
0 & 0 & \boxtimes & \boxtimes& \boxtimes \\
0 & 0 & 0 & \boxtimes& \boxtimes \\
0 & 0 & 0 & 0& \boxtimes
\end{sbmatrix}
&\stackrel{\bG_{45}}{\rightarrow}
\footnotesize
\begin{sbmatrix}{\bG_{45}\bR}
\boxtimes & \boxtimes & \boxtimes & \boxtimes& \boxtimes \\
0 & \boxtimes & \boxtimes & \boxtimes& \boxtimes \\
0 & 0 & \boxtimes & \boxtimes& \boxtimes \\
0 & 0 & 0 & \bm{\boxtimes}& \bm{\boxtimes} \\
0 & 0 & 0 & \bm{\boxtimes}& \bm{\boxtimes}
\end{sbmatrix}
\stackrel{\bG_{34}}{\rightarrow}
\footnotesize
\begin{sbmatrix}{\bG_{34}\bG_{45}\bR}
\boxtimes & \boxtimes & \boxtimes & \boxtimes& \boxtimes \\
0 & \boxtimes & \boxtimes & \boxtimes& \boxtimes \\
0 & 0 & \bm{\boxtimes} & \bm{\boxtimes}& \bm{\boxtimes} \\
0 & 0 & \bm{\boxtimes} & \bm{\boxtimes}& \bm{\boxtimes} \\
0 & 0 & 0 & \boxtimes & \boxtimes
\end{sbmatrix}\\
&\stackrel{\bG_{23}}{\rightarrow}
\footnotesize
\begin{sbmatrix}{\bG_{23}\bG_{34}\bG_{45}\bR}
\boxtimes & \boxtimes & \boxtimes & \boxtimes& \boxtimes \\
0 & \bm{\boxtimes} & \bm{\boxtimes} & \bm{\boxtimes}& \bm{\boxtimes} \\
0 & \bm{\boxtimes} & \bm{\boxtimes} & \bm{\boxtimes}& \bm{\boxtimes} \\
0 & 0 & \boxtimes & \boxtimes& \boxtimes \\
0 & 0 & 0 & \boxtimes& \boxtimes
\end{sbmatrix}
\stackrel{\bG_{12}}{\rightarrow}
\footnotesize
\begin{sbmatrix}{\bG_{12}\bG_{23}\bG_{34}\bG_{45}\bR}
\bm{\boxtimes} & \bm{\boxtimes} & \bm{\boxtimes} & \bm{\boxtimes}& \bm{\boxtimes} \\
\bm{\boxtimes} & \bm{\boxtimes} & \bm{\boxtimes} & \bm{\boxtimes}& \bm{\boxtimes} \\
0 & \boxtimes & \boxtimes & \boxtimes& \boxtimes \\
0 & 0 & \boxtimes & \boxtimes& \boxtimes \\
0 & 0 & 0 & \boxtimes& \boxtimes
\end{sbmatrix}.
\end{aligned}
$$
And the forward rotations result in a \textbf{full (non-sparse) matrix}:
$$
\begin{aligned}
\text{\parbox{7em}{Forward\\(Wrong Way)}: }
\footnotesize
\begin{sbmatrix}{\bR}
\boxtimes & \boxtimes & \boxtimes & \boxtimes& \boxtimes \\
0 & \boxtimes & \boxtimes & \boxtimes& \boxtimes \\
0 & 0 & \boxtimes & \boxtimes& \boxtimes \\
0 & 0 & 0 & \boxtimes& \boxtimes \\
0 & 0 & 0 & 0& \boxtimes
\end{sbmatrix}
&\stackrel{\bG_{12}}{\rightarrow}
\footnotesize
\begin{sbmatrix}{\bG_{12}\bR}
\bm{\boxtimes} & \bm{\boxtimes} & \bm{\boxtimes} & \bm{\boxtimes}& \bm{\boxtimes} \\
\bm{\boxtimes} & \bm{\boxtimes} & \bm{\boxtimes} & \bm{\boxtimes}& \bm{\boxtimes} \\
0 & 0 & \boxtimes & \boxtimes& \boxtimes \\
0 & 0 & 0 & \boxtimes& \boxtimes \\
0 & 0 & 0 & 0& \boxtimes
\end{sbmatrix}
\stackrel{\bG_{23}}{\rightarrow}
\footnotesize
\begin{sbmatrix}{\bG_{23}\bG_{12}\bR}
\boxtimes & \boxtimes & \boxtimes & \boxtimes& \boxtimes \\
\bm{\boxtimes} & \bm{\boxtimes} & \bm{\boxtimes} & \bm{\boxtimes}& \bm{\boxtimes} \\
\bm{\boxtimes} & \bm{\boxtimes} & \bm{\boxtimes} & \bm{\boxtimes}& \bm{\boxtimes} \\
0 & 0 & 0 & \boxtimes& \boxtimes \\
0 & 0 & 0 & 0& \boxtimes
\end{sbmatrix}\\
&\stackrel{\bG_{34}}{\rightarrow}
\footnotesize
\begin{sbmatrix}{\bG_{34}\bG_{23}\bG_{12}\bR}
\boxtimes & \boxtimes & \boxtimes & \boxtimes& \boxtimes \\
\boxtimes & \boxtimes & \boxtimes & \boxtimes& \boxtimes \\
\bm{\boxtimes} & \bm{\boxtimes} & \bm{\boxtimes} & \bm{\boxtimes}& \bm{\boxtimes} \\
\bm{\boxtimes} & \bm{\boxtimes} & \bm{\boxtimes} & \bm{\boxtimes}& \bm{\boxtimes} \\
0 & 0 & 0 & 0& \boxtimes
\end{sbmatrix}
\stackrel{\bG_{45}}{\rightarrow}
\footnotesize
\begin{sbmatrix}{\bG_{45}\bG_{34}\bG_{23}\bG_{12}\bR}
\boxtimes & \boxtimes & \boxtimes & \boxtimes& \boxtimes \\
\boxtimes & \boxtimes & \boxtimes & \boxtimes& \boxtimes \\
\boxtimes & \boxtimes & \boxtimes & \boxtimes& \boxtimes \\
\bm{\boxtimes} & \bm{\boxtimes} & \bm{\boxtimes} & \bm{\boxtimes}& \bm{\boxtimes} \\
\bm{\boxtimes} & \bm{\boxtimes} & \bm{\boxtimes} & \bm{\boxtimes}& \bm{\boxtimes} \\
\end{sbmatrix}.
\end{aligned}
$$
That is,  backward rotations will preserve a lot of the zeros as they are, whereas the forward rotations will eliminate these zeros.
In general,  backward rotations yield
$$
\bG_{12} \bG_{23} \ldots \bG_{(p-1),p} (\bR+\bw\bv^\top) = \bH_0  \pm \normtwo{\bw} \be_1 \bv^\top \triangleq \bH, 
$$
which is  in upper Hessenberg form. 
We can thus find a sequence of rotations $\bJ_{12}, \bJ_{23}, \ldots, \bJ_{(p-1),p}$ such that
$$
\bJ_{(p-1),p} \ldots \bJ_{23}\bJ_{12}\bH \triangleq \bR^\prime
$$
is upper triangular. 
Following the example of the $5\times 5$ matrix, the triangularization process is presented below:
$$
\begin{aligned}
\underbrace{\bH_0  \pm \normtwo{\bw} \be_1 \bv^\top}_{\bH} =
\footnotesize
\begin{sbmatrix}{\bH}
\boxtimes & \boxtimes & \boxtimes & \boxtimes & \boxtimes \\
\boxtimes & \boxtimes & \boxtimes & \boxtimes & \boxtimes \\
0 & \boxtimes & \boxtimes & \boxtimes& \boxtimes \\
0 & 0 & \boxtimes & \boxtimes& \boxtimes \\
0 & 0 & 0 & \boxtimes& \boxtimes
\end{sbmatrix}
&\stackrel{\bJ_{12}}{\rightarrow}
\footnotesize
\begin{sbmatrix}{\bJ_{12}\bH}
\bm{\boxtimes} & \bm{\boxtimes} & \bm{\boxtimes} & \bm{\boxtimes} & \bm{\boxtimes} \\
\bm{0} & \bm{\boxtimes} & \bm{\boxtimes} & \bm{\boxtimes} & \bm{\boxtimes} \\
0 & \boxtimes & \boxtimes & \boxtimes& \boxtimes \\
0 & 0 & \boxtimes & \boxtimes& \boxtimes \\
0 & 0 & 0 & \boxtimes& \boxtimes
\end{sbmatrix}
\stackrel{\bJ_{23}}{\rightarrow}
\footnotesize
\begin{sbmatrix}{\bJ_{23}\bJ_{12}\bH}
\boxtimes & \boxtimes & \boxtimes & \boxtimes & \boxtimes \\
0 & \bm{\boxtimes} & \bm{\boxtimes} & \bm{\boxtimes} & \bm{\boxtimes} \\
0 & \bm{0} & \bm{\boxtimes} & \bm{\boxtimes} & \bm{\boxtimes} \\
0 & 0 & \boxtimes & \boxtimes& \boxtimes \\
0 & 0 & 0 & \boxtimes& \boxtimes
\end{sbmatrix}\\
&\stackrel{\bJ_{34}}{\rightarrow}
\footnotesize
\begin{sbmatrix}{\bJ_{34}\bJ_{23}\bJ_{12}\bH}
\boxtimes & \boxtimes & \boxtimes & \boxtimes & \boxtimes \\
0 & \boxtimes & \boxtimes & \boxtimes & \boxtimes \\
0 & 0 & \bm{\boxtimes} & \bm{\boxtimes} & \bm{\boxtimes} \\
0 & 0 & \bm{0} & \bm{\boxtimes} & \bm{\boxtimes} \\
0 & 0 & 0 & \boxtimes& \boxtimes
\end{sbmatrix}
\stackrel{\bJ_{45}}{\rightarrow}
\footnotesize
\begin{sbmatrix}{\bJ_{45}\bJ_{34}\bJ_{23}\bJ_{12}\bH}
\boxtimes & \boxtimes & \boxtimes & \boxtimes & \boxtimes \\
0 & 0 & \boxtimes & \boxtimes & \boxtimes \\
0 & 0 & \boxtimes & \boxtimes& \boxtimes \\
0 & 0 & 0 & \bm{\boxtimes} & \bm{\boxtimes} \\
0 & 0 & 0 & \bm{0} & \bm{\boxtimes} \\
\end{sbmatrix}.
\end{aligned}
$$
And the QR decomposition of $\bX^\prime$ is then given by 
$$
\bX^\prime = \bQ^\prime \bR^\prime,
$$
where 
\begin{equation}\label{equation:qr-rank-one-update}
\left\{
\begin{aligned}
\bR^\prime &\triangleq(\bJ_{(p-1),p} \ldots \bJ_{23}\bJ_{12}) (\bG_{12} \bG_{23} \ldots \bG_{(p-1),p}) (\bR+\bw\bv^\top);\\
\bQ^\prime &\triangleq \bQ\left\{(\bJ_{(p-1),p} \ldots \bJ_{23}\bJ_{12}) (\bG_{12} \bG_{23} \ldots \bG_{(p-1),p}) \right\}^\top; \\
\text{(or) }\bQ^{\prime\top}&\triangleq \left\{(\bJ_{(p-1),p} \ldots \bJ_{23}\bJ_{12}) (\bG_{12} \bG_{23} \ldots \bG_{(p-1),p}) \right\}\bQ^\top .
\end{aligned}
\right.
\end{equation}
The procedure is outlined in Algorithm~\ref{alg:qr-rankoneChange}.

\begin{algorithm}[htp] 
\caption{QR Rank-One Changes} 
\label{alg:qr-rankoneChange} 
\begin{algorithmic}[1] 
\Require Matrix $\bX \in \real^{p\times p}$ with QR decomposition $\bX=\bQ\bR$, and $\bX^\prime = \bX+\bu\bv^\top$; 
\Statex \textbf{Stage A: Transfer $\bw$ to a first basis vector, $\bw\rightarrow \normtwo{\bw}\be_1$}
\State Calculate $\bw\leftarrow\bQ^\top\bu$; 
%\Comment{$n(2p-1)$ flops}
\State Calculate $\bH \leftarrow \bR$;
\For{$i=p-1$ to $1$} 
\State Get Givens rotation $\bG_{i,i+1}$ with the following parameters $c, s$:
\State $c \leftarrow \frac{x_k}{\sqrt{x_k^2 + x_l^2}}$, $s\leftarrow\frac{x_l}{\sqrt{x_k^2 + x_l^2}}$, where $x_k = \bw_i$, $x_l = \bw_{i+1}$; 
%\Comment{6 flops}
\State Calculate $\bH \leftarrow \bG_{i,i+1}\bH $ in following two steps:
\State $i$-th row: $\bH_{i,:} \leftarrow c\cdot \bH_{i,:} + s\cdot \bH_{j,:} $, where $j=i+1$; 
%\Comment{$3(n-i+1)$ flops}
\State $(i+1)$-th row: $\bH_{i+1,:} \leftarrow -s\cdot \bH_{i,:} + c\cdot \bH_{j,:} $, where $j=i+1$; 
%\Comment{$3(n-i+1)$ flops}
\EndFor

\Statex \textbf{Stage B: Triangularize $\bR^\prime$}
\State Set $\bR^\prime \leftarrow\bH \pm \normtwo{\bw} \be_1 \bv^\top$; 
\Comment{$\bH, \bR^\prime$ are both upper Hessenberg}
\For{$i=1$ to $p-1$} 
\State Get Givens rotation $\bJ_{i,i+1}$ with the following parameters $c, s$:
\State $c \leftarrow \frac{x_k}{\sqrt{x_k^2 + x_l^2}}$, $s\leftarrow\frac{x_l}{\sqrt{x_k^2 + x_l^2}}$, where $x_k = \bH_{i,i}$, $x_l = \bH_{i+1,i}$;
\State Calculate $\bR^\prime \leftarrow \bJ_{i,i+1}\bR^\prime $ in following two steps:
\State $i$-th row: $\bR^\prime_{i,:} \leftarrow c\cdot \bR^\prime_{i,:} + s\cdot \bR^\prime_{j,:} $, where $j=i+1$;
\State $(i+1)$-th row: $\bR^\prime_{i+1,:} \leftarrow -s\cdot \bR^\prime_{i,:} + c\cdot \bR^\prime_{j,:} $, where $j=i+1$;
\EndFor
\State Output $\bR^\prime$; 

\Statex \textbf{Stage C: Obtain the orthogonal matrix $\bQ^\prime$}
\State Set $\bQ^{\prime\top} = \bQ^\top$;
\For{$i=p-1$ to $1$}    \Comment{The following $c,s$ are from step 5}
\State $i$-th row: $\bQ^{\prime\top}_{i,:} \leftarrow c\cdot \bQ^{\prime\top}_{i,:} + s\cdot \bQ^{\prime\top}_{j,:} $, where $j=i+1$;  
%\Comment{$6n$ flops}
\State $(i+1)$-th row: $\bQ^{\prime\top}_{i+1,:} \leftarrow -s\cdot \bQ^{\prime\top}_{i,:} + c\cdot \bQ^{\prime\top}_{j,:} $, where $j=i+1$; 
%\Comment{$6n$ flops}
\EndFor
\For{$i=1$ to $p-1$} \Comment{The following $c,s$ are from step 13}
\State $i$-th row: $\bQ^{\prime\top}_{i,:} \leftarrow c\cdot \bQ^{\prime\top}_{i,:} + s\cdot \bQ^{\prime\top}_{j,:} $, where $j=i+1$; 
%\Comment{$6n$ flops}
\State $(i+1)$-th row: $\bQ^{\prime\top}_{i+1,:} \leftarrow -s\cdot \bQ^{\prime\top}_{i,:} + c\cdot \bQ^{\prime\top}_{j,:} $, where $j=i+1$;
%\Comment{$6n$ flops}
\EndFor
\State Output $\bQ^\prime$;

\end{algorithmic} 
\end{algorithm}

We state the complexity of the rank-one update in the following theorem.
\begin{theoremHigh}[Algorithm complexity: QR rank-one change \citep{lu2021numerical}]\label{theorem:qr-full-givens-rank1}
Algorithm~\ref{alg:qr-rankoneChange} requires $\sim 8p^2$ flops to compute the full QR decomposition of an $\bX^\prime \in \real^{p\times p}$ matrix with a rank-one change to $\bX$, given that the full QR decomposition of $\bX$ is already known. Furthermore, if the orthogonal matrix $\bQ^\prime$ needs to be formed explicitly, an additional $\sim 12p^2$ flops are required.
\end{theoremHigh}

The algorithm can be easily applied to a rectangular matrix $\bX\in \real^{n\times p}$ or to the sum  $\bX+\bU\bV^\top$, where $\bU\in \real^{n\times k}$ and $\bV\in \real^{p\times k}$; see \citet{golub2013matrix} for more details.

\subsection{Rank-Deficient Case}

We now show that there exists a column permutation matrix $\bP$ such that, in the QR decomposition of $\bX \bP$, all diagonal elements equal to zero appear at the end.
This is referred to as the \textit{column-pivoted QR (CPQR)} decomposition, or alternatively, the \textit{rank-deficient QR} decomposition.
\begin{theoremHigh}[Column-pivoted QR decomposition]\label{theorem:rank-revealing-qr-general}
Ler $\bX  \in \real^{n\times p}$ with rank $r$ such that $r < p \leq n$. 
Then there exist a permutation matrix $\bP$ and an orthogonal matrix $\bQ \in \real^{n \times n}$ such that
$$
\bX \bP = \bQ \begin{bmatrix}
\bR_{11} & \bR_{12} \\
\bzero & \bzero
\end{bmatrix},
$$
where $\bR_{11} \in \real^{r \times r}$ is upper triangular with positive diagonal elements.
\end{theoremHigh}
\begin{proof}[of Theorem~\ref{theorem:rank-revealing-qr-general}]
Since $\rank(\bX) = r$, we can always choose a permutation matrix $\bP$ such that $\bX \bP = [\bX_1 , \bX_2]$, where $\bX_1 \in \real^{n \times r}$ has linearly independent columns. The QR decomposition
$$
\bQ^\top \bX_1 = \begin{bmatrix}
\bR_{11} \\
\bzero
\end{bmatrix}, \quad \bQ = \begin{bmatrix}
\bQ_1 & \bQ_2
\end{bmatrix}
$$
uniquely determines $\bQ_1 \in \real^{n \times r}$ and $\bR_{11} \in \real^{r \times r}$ with positive diagonal elements. Then
$$
\bQ^\top \bX \bP = \begin{bmatrix}
\bQ^\top \bX_1 & \bQ^\top \bX_2
\end{bmatrix} = \begin{bmatrix}
\bR_{11} & \bR_{12} \\
\bzero & \bzero
\end{bmatrix}
$$
has rank $r$. Here $\bR_{22} = \bzero$, because $\bR$ cannot have more than $r$ linearly independent rows. 
This completes the proof.
\end{proof}

From the CPQR decomposition and orthogonal invariance  of the $\ell_2$ norm, it follows that the least squares problem $\min_{\bbeta} \normtwo{\bX\bbeta - \by}$ is equivalent to
\begin{equation}\label{equation:ls_cpqr}
\min_{\bbeta} \normtwo{\bQ^\top\bX\bP\bP^\top\bbeta - \bQ^\top\by}
\equiv
\min_{\widetildebbeta} \normtwo{\begin{bmatrix}
\bR_{11} & \bR_{12} \\
\bzero & \bzero
\end{bmatrix} \begin{bmatrix}
\widetildebbeta_1 \\
\widetildebbeta_2
\end{bmatrix} - \begin{bmatrix}
\bd_1 \\
\bd_2
\end{bmatrix}},
\end{equation}
where 
$\bX \bP = \bQ 
\footnotesize
\begin{bmatrix}
\bR_{11} & \bR_{12} \\
\bzero & \bzero
\end{bmatrix}$ is the CPQR of $\bX\in\real^{n\times p}$, $\bd \triangleq \bQ^\top \by$ and $\widetildebbeta \triangleq \bP^\top\bbeta$ are partitioned conformally. The general solution of \eqref{equation:ls_cpqr} is given by $\bbeta = \bP \footnotesize\begin{bmatrix} \widetildebbeta_1 \\ \widetildebbeta_2 \end{bmatrix}$, where
$
\bR_{11} \widetildebbeta_1 + \bR_{12} \widetildebbeta_2= \bd_1 ,
$
and $\widetildebbeta_2$ can be chosen arbitrarily. For $\widetildebbeta_2 = \bzero$, we obtain a {basic least squares solution} 
$\widetildebbeta = [\widetildebbeta_r^\top, \bzero]^\top$, and
\begin{equation}\label{equation:basic_sol_cpqr}
\widehatbbeta_r = \bP 
\begin{bmatrix} \widetildebbeta_r \\ 
	\bzero \end{bmatrix}, 
\qquad \widetildebbeta_r 
\triangleq \bR_{11}^{-1} \bd_1,
\end{equation}
with at most $r = \rank(\bX)$ nonzero components. The general solution in terms of $\widetildebbeta_2\in\real^{p-r}$ (which can vary) is given by
\begin{equation}\label{equation:gen_sol_cpqr}
\widehatbbeta = \bP 
\begin{bmatrix} 
\widetildebbeta_r - \bZ \widetildebbeta_2 \\
\widetildebbeta_2 
\end{bmatrix}, 
\qquad \bZ \triangleq \bR_{11}^{-1} \bR_{12},
\end{equation}
where $\bZ$ can be computed in  $\sim r^2(p-r)$ flops by solving the matrix equation $\bR_{11} \bZ = \bR_{12}$ using backward substitution.

\paragrapharrow{Smoothing LS and minimum-norm problems.}
Another general approach to address rank deficiency is to find the solution to the least squares problem
\begin{equation}\label{equation:smoo_ls_cpqr}
\min_{\bbeta \in \sB} \normtwo{\bB \bbeta}, \quad\sB \triangleq \{ \bbeta \in \real^p \mid \normtwo{\by - \bX\bbeta} = \min \},
\end{equation}
where the matrix $\bB$ can be chosen so that $\normtwo{\bB\bbeta}$ is a measure of the smoothness of $\bbeta$. 
Similar to the full-rank case \eqref{equation:qrfulo_minnom}, 
substituting the general solution \eqref{equation:gen_sol_cpqr} shows that the solution to \eqref{equation:smoo_ls_cpqr} is given by
\begin{equation}\label{equation:smooth_sol_cpqr}
\min_{\bb\in\real^{p-r}} 
\normtwo{\bB
\begin{bmatrix} \bZ \\ 
-\bI_{p-r} 
\end{bmatrix} \bb 
- \bB
\begin{bmatrix} 
\widetildebbeta_r \\ 
\bzero
\end{bmatrix}}.
\end{equation}
In particular, taking $\bB= \bI$ reduces to the minimum-norm problem, which minimizes
$$
\normtwo{\bbeta}^2 =\normtwo{\bZ \bb}^2 + \normtwo{\bb}^2,
$$
and gives the pseudo-inverse solution $\widehatbbeta = \bX^+\by$. 
It can be shown that 
$$
\nspace(\bX \bP) = \cspace \left(\begin{bmatrix} \bZ \\ -\bI_{p-r} \end{bmatrix}\right)
$$
forms a (nonorthonormal) basis for $\nspace(\bX \bP)$. QR factorization gives an orthonormal basis for $\nspace(\bX \bP)$. Note that the unique pseudo-inverse solution orthogonal to $\nspace(\bX \bP)$ (see Theorem~\ref{theorem:unif_ls}) equals the residual of the least squares problem \eqref{equation:smooth_sol_cpqr} with $\bB= \bI$,
$$
\widehatbbeta = 
\bP
\begin{bmatrix} 
\widetildebbeta_r \\ 
\bzero \end{bmatrix} - 
\bP\begin{bmatrix} 
\bZ \\ 
-\bI_{p-r} 
\end{bmatrix} 
\widehatbb,
\quad \text{with} \quad 
\widehatbb 
=
\argmin_{\bb}\normtwo{\bZ \bb}^2 + \normtwo{\bb}^2.
$$
Observe that this expression takes the form of the basic solution minus a correction term lying in the null space of  $\bX \bP$. Any particular solution can be substituted for $\bb$ in \eqref{equation:smooth_sol_cpqr}.

\subsection*{Computing the CPQR Decomposition}\label{section:piv_qr}

For many applications it is preferable to use a column-pivoted QR factorization (CPQR), in which the pivot column at step $k$ is chosen to maximize the diagonal element $r_{kk}$.

We now explain how to implement this strategy using the Modified Gram–Schmidt (MGS) process. 
Suppose that after $(k-1)$ steps, the nonpivotal columns are transformed according to
$$
\bx_j^{(k)} = \bx_j - \sum_{i=1}^{k-1} r_{ij} \bq_i, \quad j = k, \ldots, p,
$$
where $\bx_j^{(k)}$ is orthogonal to $\cspace(\bX_{k-1}) = \text{span}\{\bq_1, \ldots, \bq_{k-1}\}$. 
In the $k$-th step we select $s$, so that
\begin{equation}\label{equation:cpqr_max}
\normtwo{\bx_s^{(k)}}^2 = \max_{k \leq j \leq p} \normtwo{\bx_j^{(k)}}^2,
\end{equation}
and interchange columns $k$ and $s$. This is equivalent to choosing at the $k$-th step a pivot column with largest distance to the subspace $\cspace(\bX_{k-1}) = \text{span}\{\bx_{s_1},\bx_{s_2}, \ldots, \bx_{s_{k-1}}\}$, where $\bX_{k-1}$ is the submatrix formed by the columns corresponding to the first $k-1$ selected pivots.

\paragrapharrow{MGS CPQR.}
Building on the recursive MGS algorithm described in Algorithm~\ref{alg:qr-mgs-fulll-rowwise-recursive}, we can also develop a practical implementation of CPQR. This variant is presented in Algorithm~\ref{alg:cpqr-partial-fact-cpqr}. The only difference from the standard MGS lies in the highlighted \textcolor{mylightbluetext}{blue} portion: before each iteration, we permute the column with the largest norm into the leading position.

As a result, we obtain a triangular matrix $\bR$ satisfying
\begin{subequations}
\begin{equation}
r_{11} \geq r_{22} \geq \ldots \geq r_{rr}.
\end{equation}
Since $r_{kk}=\normtwobig{\bx_s^{(k)}}$,  and $\normtwobig{\bx_s^{(k)}}^2 = \max_{k \leq j \leq p} \normtwobig{\bx_j^{(k)}}^2$ by \eqref{equation:cpqr_max}. Therefore, $r_{k,k+1}^2 + r_{k+1,k+1}^2 \leq \normtwobig{\bx_s^{(k)}}^2\equiv r_{kk}^2$.
This argument recursively shows that the diagonal elements in $\bR$ satisfy the stronger inequalities
\begin{equation}\label{equation:cpqr_strong_ineq}
r_{kk}^2 \geq \sum_{i=k}^j r_{ij}^2, \quad j = k+1, \ldots, p, \quad k = 1 : r.
\end{equation}
This implies that if $r_{kk} = 0$, then $r_{ij} = 0$, $i, j \geq k$. In particular,
$$
\abs{r_{11}} = \max_{1 \leq j \leq p} \left\{\abs{\be_j^\top \bR \be_1} \mid \bX \bP_{1j} = \bQ\bR\right\},
$$
where $\bP_{1j}$ is the permutation matrix that interchanges columns 1 and $j$. Then $\normf{\bX}^2 \leq p r_{11}^2$, which yields upper and lower bounds for $\sigma_1(\bX)$,
\begin{equation}\label{equation:cpqr_ineq_condi}
\abs{r_{11}} \leq \sigma_1(\bX) \leq \sqrt{p} \abs{r_{11}}.
\end{equation}
\end{subequations}
If a diagonal element $r_{kk}$ in CPQR vanishes, it follows from \eqref{equation:cpqr_strong_ineq} that $r_{ij} = 0$, $i, j \ge k$. 

Now suppose that at an intermediate stage of CPQR, the new diagonal element satisfies $r_{k+1,k+1} \leq \delta$ for some small $\delta$. Then by \eqref{equation:cpqr_ineq_condi},
$$
\normf{\bX^{(k)}} \leq \sqrt{p-k} \cdot \delta,
$$
and setting $\bX^{(k)} = \bzero$ corresponds to a perturbation $\bE_k$ of $\bX$, such that $\bX + \bE_k$ has rank-$k$ and $\normf{\bE_k} \leq \sqrt{p-k}\cdot \delta$. The matrix
$$
\widehatbX = \bQ_1 \begin{bmatrix}
\bR_{11} & \bR_{12}
\end{bmatrix} \bP^\top, \quad \bQ 
= 
\begin{bmatrix}
\bQ_1 & \bQ_2
\end{bmatrix},
$$
obtained by neglecting $\bR_{22}$, is the best rank-$k$ approximation to $\bX$ that differs from $\bX \bP$ only in the last $p-k$ columns. In particular, when $k = p-1$, we get $\normfbig{\widehatbX - \bX}= r_{pp}$ \citep{bjorck2024numerical}.

A commonly used stopping criterion for CPQR is to terminate the process when $r_{k+1,k+1} \leq \delta$.
However, this may significantly overestimate the numerical rank of $\bX$. It can be shown that  
$$
\sigma_p \ge \frac{3\abs{r_{pp}}}{\sqrt{4^p + 6p - 1}} \ge 2^{1-p} \abs{r_{pp}}.
$$
This inequality demonstrates that $\sigma_p$ can be much smaller than $\abs{r_{pp}}$ for moderately large values of $p$ \citep{faddeev1968solution}.

\begin{algorithm}[htp] 
\caption{\textcolor{mylightbluetext}{Practical} CPQR via MGS (\textcolor{mylightbluetext}{Row-Wise and Recursively}). 
The algorithm is derived from Algorithm~\ref{alg:qr-mgs-fulll-rowwise-recursive} and a similar procedure can be derived based on Algorithm~\ref{alg:mgs_rowwise-recursive_comp} and Algorithm~\ref{alg:mgs_recursive_comp_moreortho}.
}
\label{alg:cpqr-partial-fact-cpqr}
\begin{algorithmic}[1] 
\Require $\bX\in \real^{n\times p}$ with  rank $r$;
\For{$k=1$ to $p$}  \Comment{i.e., compute $k$-th column of $\bQ$ and $k$-th row of $\bR$}
\State \textcolor{mylightbluetext}{Find the column with largest norm in $\bX$, and permute to first column;}
\State $\bx_1\leftarrow\bX[:,1]$; \Comment{$1$-st column of $\bX\in \real^{n\times (p-k+1)}$}
\State $r_{kk}\leftarrow\normtwo{\bx_1}$;\Comment{$\bx_1\in \real^{n\times 1}$}
\State $\bq_k \leftarrow \bx_1/r_{kk}$;
\State $\br_{k2}^\top\leftarrow\bq_k^\top\bX_2$; \Comment{$\bX_2\triangleq\bX[:,2:p]\in \real^{n\times (p-k)}$, $\br_{k2}^\top\in \real^{1\times (p-k)}$}
\State $\bX\leftarrow\bX_2-\bq_k\br_{k2}^\top$; \Comment{$\bX \in \real^{n\times (p-k)}$}
\State \textcolor{mylightbluetext}{Exit when $r_{kk}=0$ or $r_{kk}<\delta$;}
\EndFor 
\State Output permutations, $\bQ=[\bq_1, \ldots, \bq_p]$ and $\bR$ with entry $(i,k)$ being $r_{ik}$.
\end{algorithmic} 
\end{algorithm}

\paragrapharrow{Reduction in computational cost.}
Note that, in each iteration, we need to calculate  the norms of all the (remaining) columns of $\bX$ rather than  computing the norms all at once. At iteration $k$, we need to compute the reduced QR decomposition of a matrix of size $n\times (p-k+1)$ if the original matrix $\bX$ is of size $n\times p$. That is, extra $(p-k+1)(2n-1)$ flops are required to proceed with the CPQR via MGS. Let $f(k)=(p-k+1)(2n-1)$; simple calculation can show that the additional complexity for CPQR via MGS is:
\begin{equation}\label{equation:mgs-cpqr-extra1}
\text{extra cost = }f(1)+f(2)+\ldots +f(p)\sim np^2 \text{  flops},
\end{equation}
if only keep the leading term. 

However, this additional cost in CPQR via MGS can be mitigated to some extent.
Suppose the column partition of $\bX\in \real^{n\times p}$ is $\bX=[\bx_1, \bx_2, \ldots, \bx_p]$, and let the squared norm of each column be given in the vector 
$$
\bl_a = 
\begin{bmatrixfoot}
l_1 \\
l_2 \\
\vdots \\
l_p
\end{bmatrixfoot}=
\begin{bmatrixfoot}
\normtwo{\bx_1}^2\\
\normtwo{\bx_2}^2\\
\vdots \\
\normtwo{\bx_p}^2\\
\end{bmatrixfoot}.
$$
Suppose further that $\bq\in \real^n$ is a unit-length vector such that $\bq^\top\bq=1$, and $\br\in \real^p$ is a vector given by 
$$
\br = \bX^\top\bq
=
\begin{bmatrixfoot}
r_1 \\
r_2 \\
\vdots \\
r_p
\end{bmatrixfoot}
. \gap \text{(similar to the step 6 of  Algorithm~\ref{alg:cpqr-partial-fact-cpqr})}
$$
Let further $\bB=\bX-\bq\br^\top=[\bb_1, \bb_2, \ldots, \bb_p]$ (similar to the step 7 of  Algorithm~\ref{alg:cpqr-partial-fact-cpqr}). 
The vector representing the squared lengths of $\bB$ is given by
$$
\bl_b = 
\begin{bmatrixfoot}
s_1 \\
s_2 \\
\vdots \\
s_p
\end{bmatrixfoot}=
\begin{bmatrixfoot}
\normtwo{\bb_1}^2\\
\normtwo{\bb_2}^2\\
\vdots \\
\normtwo{\bb_p}^2\\
\end{bmatrixfoot}
=
\begin{bmatrixfoot}
l_1 -r_1^2\\
l_2 -r_2^2\\
\vdots \\
l_p-r_p^2
\end{bmatrixfoot}.
$$
This can be easily verified since $\bb_i = \bx_i -r_{i}\bq =\bx_i-(\bx_i^\top\bq) \bq$ such that 
$$
\normtwo{\bb_i}^2 = \normtwo{\bx_i -r_{i}\bq}^2=( \bx_i -r_{i}\bq)^\top( \bx_i -r_{i}\bq)=l_i-r_i^2.
$$

If the column norms $\normtwobig{\bx_j^{(k)}}$ in \eqref{equation:cpqr_max} are recomputed at each stage of MGS, this will increase the operation count of the QR factorization by 50\% \citep{bjorck2024numerical}. Since these quantities are invariant under orthogonal transformations, this overhead can be reduced to $\mathcalO(np)$ operations by using the recursion
$$
\normtwo{\bx_j^{(k+1)}}^2 = \normtwo{\bx_j^{(k)}}^2 - r_{kj}^2, \quad j = k+1, \ldots, p,
$$
to update these values. To avoid numerical problems, $\normtwobig{\bx_j^{(k)}}$ should be recomputed from scratch whenever there has been substantial cancellation, e.g., when $\normtwobig{\bx_j^{(k+1)}} \leq \normtwobig{\bx_j^{(k)}} / \sqrt{2}$.

Coming back to step 2 of Algorithm~\ref{alg:cpqr-partial-fact-cpqr}, suppose we have computed the squared norms of the columns from the original matrix $\bX\in \real^{n\times p}$ (which requires $p(2n-1)$). The squared norms of the columns from $\bX_2-\bq_1\br_{12}^\top$ (suppose $k=1$ in step 7 of Algorithm~\ref{alg:cpqr-partial-fact-cpqr}) can be obtained with an additional $2(p-1)$ flops. 
Over the course of the $p$ iterations, the total cost is $2(p-1)+2(p-2)+\ldots +2(1)=p^2-p$ flops. This is significantly less than the complexity of $\sim np^2$ in Equation~\eqref{equation:mgs-cpqr-extra1}.

\paragrapharrow{CPQR for LS.}

The column pivoting strategy described above is independent of the right-hand side vector $\by$, and therefore may not be the most suitable choice when solving a specific least squares problem. For example, suppose $\by$ is a multiple of one of the columns in $\bX$. Using standard pivoting, this situation might not be recognized until after the full QR factorization has been completed.
An alternative strategy is to select the pivot column at step $k+1$ as the column for which the current residual norm $\normtwobig{\by - \bX\bbeta^{(k)}}$ is maximally reduced. For MGS this is achieved by choosing as pivot the column $\bx_p$ that forms the smallest acute angle with the current residual vector $\be^{(k)} \triangleq \by - \bX\bbeta^{(k)}$. 
Therefore, the column is chosen to maximize
$$
\frac{\big(\bx_j^{(k)}\big)^\top \be^{(k)}}{\normtwobig{\bx_j^{(k)}} \normtwobig{\be^{(k)}}}.
$$
This ensures that each pivot contributes the most to reducing the residual in the current iteration.

\index{Elliptic MGS}
\index{GLS}
\subsection{GLS by Elliptic MGS and Householder Methods}\label{section:gls_ellipmgs}

For a given symmetric positive definite matrix $ \bPhi $,
\begin{equation}
\innerproduct{\bx, \by}_{\bPhi} = \by^\top \bPhi \bx, \qquad \norm{\bx}_{\bPhi} = (\bx^\top \bPhi \bx)^{1/2}
\end{equation}
defines a scalar inner product and the corresponding norm; see Section~\ref{section:generalizedLS}. 
Since the unit ball $\{\bx \mid \norm{\bx}_{\bPhi} \leq 1\}$ forms an ellipsoid, this norm  $\norm{\cdot}_{\bPhi}$ is also referred to as an elliptic norm. 
A generalized Cauchy-Schwarz inequality holds:
\begin{equation}
\abs{\innerproduct{\bx, \by}_{\bPhi}} \leq \norm{\bx}_{\bPhi} \norm{\by}_{\bPhi}.
\end{equation}
Two vectors $ \bx $ and $ \by $ are said to be $ \bPhi $-orthogonal if $\innerproduct{\bx, \by}_{\bPhi} = 0$, and a matrix $ \bQ \in \real^{n\times p} $ is $ \bPhi $-semi-orthogonal if $ \bQ^\top \bPhi \bQ = \bI $.

If $ \bX = [\bx_1,\bx_2, \ldots,\bx_p] \in \real^{n\times p} $ has full column rank, then an elliptic MGS algorithm can be used to compute a $ \bPhi $-semi-orthogonal matrix $ \bQ_1 = [\bq_1,\bq_2, \ldots, \bq_p] $ and an upper triangular matrix $ \bR $ such that
\begin{equation}\label{equation:g_mgs_qr}
\bX = \bQ_1 \bR, \quad \text{with}\quad \bQ_1^\top \bPhi \bQ_1 = \bI_p.
\end{equation}

\subsection*{Elliptic MGS QR Decomposition}
Similar to the elementary projector in Exercise~\ref{exercise:projection_matrix_intro2},
an \textit{elementary elliptic projector} has the form
\begin{equation}\label{equation:elem_ellip_proj}
\bP = (\bI - \bq \bq^\top \bPhi), \quad\text{with } \bq^\top \bPhi \bq = 1
\end{equation}
and satisfies $ \bP^2 = \bI - 2 \bq \bq^\top \bPhi + \bq (\bq^\top \bPhi \bq) \bq^\top \bPhi = \bP $. It is easily verified that for any vector $ \ba  $, $ \bq^\top \bPhi (\bP \ba) = \bzero $, i.e., $ \bP \ba $ is $ \bPhi $-orthogonal to $ \bq $. Note that $ \bP $ is \textbf{not} symmetric and therefore is an oblique projector; see Definition~\ref{definition:orthogonal-projection-matrix} and Section~\ref{section:prop_obli_proj}. Furthermore,
\begin{equation}
\bPhi^{1/2} \bP \bPhi^{-1/2} = \bI - \widetildebq \widetildebq^\top, \quad \text{with } \widetildebq \triangleq  \bPhi^{1/2} \bq
\end{equation}
is an orthogonal projector.

An updated MGS process can compute the factorization in \eqref{equation:g_mgs_qr}; the process is shown in Algorithm~\ref{alg:qr-mgs-right_gen}.

\noindent
\begin{minipage}[t]{0.495\linewidth}
%\begin{algorithm}[H] 
%\caption{CGS (=Algorithm~\ref{alg:reduced-qr})} 
%\label{alg:qr-mgs-left}
%\begin{algorithmic}[1] 
%\Require $\bX\in \real^{n\times p}$ with full column rank;
%\For{$k=1$ to $n$} 
%\State $\bx_k^\perp\leftarrow\bx_k$;
%\For{$i=1$ to $k-1$}
%\State $\boxed{r_{ik} \leftarrow\bq_i^\top\bx_k}$;
%\State $\bx_k^\perp\leftarrow \bx_k^\perp-r_{ik}\bq_i$; \,\,($\dagger$)
%\EndFor 
%\State $r_{kk} \leftarrow \normtwo{\bx_k^\perp}$; 
%\State $\bq_k \leftarrow \bx_k^\perp/r_{kk}$; 
%\EndFor 
%\end{algorithmic} 
%\end{algorithm}
\begin{algorithm}[H] 
\caption{MGS QR (=Algorithm~\ref{alg:qr-mgs-right})}
\label{alg:qr-mgs-right_dup}
\begin{algorithmic}[1] 
\Require $\bX\in \real^{n\times p}$ with full column rank;
\For{$k=1$ to $p$} 
\State $\bx_k^\perp\leftarrow\bx_k$;
\For{$i=1$ to $k-1$}
\State ${r_{ik} \leftarrow\bq_i^\top\textcolor{black}{\bx_k^\perp}}$;
\State $\bx_k^\perp\leftarrow \bx_k^\perp-r_{ik}\bq_i$; 
\EndFor 
\State $\boxed{r_{kk} \leftarrow \normtwo{\bx_k^\perp}}$; 
\State $\bq_k \leftarrow \bx_k^\perp/r_{kk}$; 
\EndFor 
\end{algorithmic} 
\end{algorithm}
\end{minipage}%
\hfil 
\begin{minipage}[t]{0.495\linewidth}
\begin{algorithm}[H] 
\caption{Elliptic MGS QR}
\label{alg:qr-mgs-right_gen}
\begin{algorithmic}[1] 
\Require $\bX\in \real^{n\times p}$ with full column rank;
\For{$k=1$ to $p$} 
\State $\bx_k^\perp\leftarrow\bx_k$;
\For{$i=1$ to $k-1$}
\State ${r_{ik} \leftarrow\bq_i^\top\textcolor{black}{\bx_k^\perp}}$;
\State $\bx_k^\perp\leftarrow \bx_k^\perp-r_{ik}\bq_i$; 
\EndFor 
\State $\boxed{r_{kk} \leftarrow \sqrt{ \innerproduct{\bx_k^\perp, \bPhi\bx_k^\perp}}}$; 
\State $\bq_k \leftarrow \bx_k^\perp/r_{kk}$; 
\EndFor 
\end{algorithmic} 
\end{algorithm}
\end{minipage}

In addition to the $2np^2$ flops for the standard MGS algorithms, elliptic MGS requires $2n^2p$ flops for $p$ matrix-vector products with $\bPhi$. If $n \gg p$, these operations can dominate the  total computational effort. 

However, if a factorization $\bPhi = \bB^\top \bB \in \real^{n \times n}$ is known, then
$$
\norm{\bx}_{\bPhi} = (\bx^\top \bB^\top \bB \bx)^{1/2} = \normtwo{\bB \bx},
$$
and the operations with $\bPhi$ can be replaced by operations with $\bB$ and $\bB^\top$. 
And the factorization in \eqref{equation:g_mgs_qr} can be replaced by the basic QR decomposition of $\bB\bX$
\begin{equation}\label{equation:g_mgs_qr_facbb}
\bB\bX = (\bB\bQ_1) \bR, \quad \text{with}\quad (\bQ_1^\top \bB^\top) (\bB \bQ_1) = \bI_p.
\end{equation}

\paragrapharrow{GLS using elliptic MGS.}
The GLS problem  $\min_{\bbeta}\norm{\by-\bX\bbeta}_{\bPhi}$ with $\bPhi\triangleq \bOmega^{-1}$ in \eqref{equation:gls_prob_loss} can be solved by an elliptic MGS QR decomposition.
If applied to the extended matrix $[\bX,\by]$, this gives the factorization:
\begin{equation}
\begin{bmatrix}
\bX & \by
\end{bmatrix}
=
\begin{bmatrix}
\bQ_1 & \bq_{p+1}
\end{bmatrix}
\begin{bmatrix}
\bR & \bz \\
\bzero & \rho
\end{bmatrix}.
\end{equation}
It follows that $\bX\bbeta - \by = \bQ_1(\bR\bbeta - \bz) - \rho \bq_{p+1}$, where $\bq_{p+1}$ is $\bPhi$-orthogonal to $\bQ_1$. Hence $\norm{\by - \bX\bbeta}_{\bPhi}$ is minimized when $\bR\bbeta = \bz$, and the solution and residual are given by
\begin{equation}
\bR\bbeta = \bz, \qquad \be = \by - \bX\bbeta = \rho \bq_{p+1}.
\end{equation}

\index{Elliptic Householder QR}
\subsection*{Elliptic Householder QR Decomposition}
Similarly, we can obtain an \textit{elliptic Householder QR factorization} \citep{gulliksson1992modifying}. 
We begin by defining the \textit{elliptic Householder reflection} matrix as follows.
\begin{definition}[Elliptic Householder reflector\index{Householder reflector}]\label{definition:ellip_house_reflector}
Let $\bu \in \real^n$ be given and $\bPhi\in\real^{n\times n}$ be positive definite. 
The matrix 
\begin{equation}
\bH = (\bI - \gamma \bu\bu^\top \bPhi) \quad \text{with }\gamma \triangleq 2 / (\bu^\top \bPhi \bu)
\end{equation}
is referred to as a \textit{an elliptic Householder reflector}, a.k.a., an \textit{elliptic Householder transformation}. 
It can be easily verified that
\begin{equation}
\bH^\top \bPhi \bH = (\bI - \gamma \bPhi \bu\bu^\top) \bPhi (\bI - \gamma \bu\bu^\top \bPhi) = \bPhi.
\end{equation}
Such matrices are called \textit{$\bPhi$-invariant}.

%\textcolor{red}{TODO}
%We call this $\bH$ the elliptic Householder reflector associated with the unit vector $\bu$, where the unit vector $\bu$ is also known as the \textit{Householder vector}. 
%When a vector $\bx$ is multiplied by $\bH$, it undergoes reflection with respect to the hyperplane $\spn\{\bu\}^\perp$.
%Note that if $\normtwo{\bu} \neq 1$, we can define the Householder reflector $\bH$ as $\bH = \bI - 2  \frac{\bu\bu^\top}{\bu^\top\bu}$.
\end{definition}

The product of an elliptic Householder reflection $\bH$ with a vector $\ba$ is given by
$$
\bH\ba = (\bI - \gamma \bu\bu^\top \bPhi)\ba = \ba - \gamma (\bu^\top \bPhi \ba) \bu. 
$$
Similar to a basic Householder reflector (Definition~\ref{definition:householder-reflector}), it can be  verified that $\bH$ is also orthogonal such that that $\bH^2 = \bI$ and  $\bH^{-1} = \bH$. However, $\bH$ is neither symmetric nor $\bPhi$-orthogonal. 
On the other hand, the transformation 
$$
\bPhi^{1/2} \bH \bPhi^{-1/2} = \bI - \gamma \widetildebu \widetildebu^\top, \quad\text{with } \widetildebu = \bPhi^{1/2} \bu 
\ \text{ and } \ \bPhi^{1/2} \bPhi^{1/2} =\bPhi
$$
yields an orthogonal reflection. 

It can be verified that the unit matrix $\bI$ is $\bPhi$-invariant, and a product of $\bPhi$-invariant matrices $\bH = \bH_1 \bH_2 \ldots \bH_p$ is again $\bPhi$-invariant. This property characterizes transformations that make the $\bPhi$-norm invariant:
\begin{align}
&\normtwo{\bQ\bbeta} = \normtwo{\bbeta}, \quad \text{if $\bQ$ is orthogonal, i.e., $\bQ$ is $\bI$-invariant}\\
&\implies 
\norm{\bH \bbeta}_{\bPhi}^2 = (\bH \bbeta)^\top \bPhi \bH \bbeta = \bbeta^\top \bPhi \bbeta = \norm{\bbeta}_{\bPhi}^2, \quad \text{if $\bH$ is $\bPhi$-invariant}. \label{equation:phi_norm_invar}
\end{align}
Therefore, $\min_{\bbeta} \norm{\bX\bbeta - \by}_{\bPhi}$ and $\min_{\bbeta} \norm{\bH(\bX\bbeta - \by)}_{\bPhi}$ have the same solution.
Using these insights, to develop a Householder QR algorithm for solving $\min_{\bbeta} \norm{\bX\bbeta - \by}_{\bPhi}$, we construct a sequence of elliptic Householder reflectors $\bH_i$ such that
\begin{equation}
\bH_p \ldots\bH_2 \bH_1 (\bX\bbeta - \by) =
\begin{bmatrix}
\bR_1  \\
\bzero 
\end{bmatrix}
\bbeta -
\begin{bmatrix}
\bc_1 \\
\bc_2
\end{bmatrix},
\end{equation}
where $\bR_1$ is upper triangular and nonsingular. 
Therefore, an equivalent problem of the GLS problem $\min_{\bbeta} \norm{\bX\bbeta - \by}_{\bPhi}$ is $\min_{\bbeta} \norm{\bR_1\bbeta - \bc_1}_{\bPhi}$ with solution $\bbeta = \bR_1^{-1} \bc_1$. As in the standard Householder method, this only requires that we construct an elliptic Householder reflector $\bH$ that maps a given vector $\bx$ onto a multiple of the unit vector $\be_1$:
\begin{equation}\label{equation:ellip_hou_qr}
\bH\bx = \bx - \gamma (\bu^\top \bPhi \bx) \bu = \pm \sigma \be_1, \quad \text{with }\sigma \triangleq \frac{\norm{\bx}_{\bPhi}}{\norm{\be_1}_{\bPhi}}.
\end{equation}
By the invariance of the $\bPhi$-norm \eqref{equation:phi_norm_invar},
$ \sigma \norm{\be_1}_{\bPhi} = \norm{\bx}_{\bPhi}, \ \norm{\be_1}_{\bPhi} = (\be_1^\top \bPhi \be_1)^{1/2}, $
and from \eqref{equation:ellip_hou_qr}, we have $\bu = \bx \pm \sigma \be_1$. Hence $\gamma = 2 / (\bu^\top \bPhi \bu)$, where
$ \bu^\top \bPhi \bu = (\bx \pm \sigma \be_1)^\top \bPhi (\bx \pm \sigma \be_1) = 2 (\norm{\bx}_{\bPhi}^2 \pm \sigma \bx^\top \bPhi \be_1). $
For stability, the sign should be chosen to maximize $\bu^\top \bPhi \bu$.

\section{LS via UTV Decomposition for Rank-Deficient Matrix}\label{section:utv_ls}

The CPQR decomposition of a rank-deficient matrix $\bX \in \real^{n\times p}$ with $\rank(\bX)=r$ is 
$
\bX\bP = 
[\bQ_1,\bQ_2]
\begin{bmatrixfoot}
\bR_{11} & \bR_{12} \\
\bzero   & \bzero 
\end{bmatrixfoot}
$, where $\bR_{11}\in\real^{r\times r}$ $r < p$, is nonsingular. 
Here, $\bQ_1$ and $\bQ_2$ provide orthogonormal bases for $\cspace(\bX)$ and $\nspace(\bX^\top)$, respectively. 
However, this factorization is not as useful in applications that require a basis for the nullspace $\nspace(\bX)$.
A related decomposition, known as the \textit{complete orthogonal decomposition},  expresses a matrix using two orthogonal matrices. It is closely related to the CPQR decomposition.
\begin{theoremHigh}[Complete orthogonal decomposition]\label{theorem:complete-orthogonal-decom}
Let $\bX\in\real^{n\times p}$ be given  with rank $r$. Then it  can be factored as 
$$
\bX = \bU \begin{bmatrix}
\bT & \bzero \\
\bzero & \bzero 
\end{bmatrix}\bV,
$$
where $\bU\in \real^{n\times n}$ and $\bV\in \real^{p\times p}$ are two orthogonal matrices, and $\bT\in \real^{r\times r}$ is a rank-$r$ matrix.
\end{theoremHigh}
\begin{proof}[of Theorem~\ref{theorem:complete-orthogonal-decom}]
By utilizing the column-pivoted QR decomposition (Theorem~\ref{theorem:rank-revealing-qr-general}), $\bX$ can be factored as 
$
\bQ_1^\top \bX\bP = 
\begin{bmatrixfoot}
\bR_{11} & \bR_{12} \\
\bzero   & \bzero 
\end{bmatrixfoot},
$
where $\bR_{11} \in \real^{r\times r}$ is upper triangular, $\bR_{12} \in \real^{r\times (p-r)}$, $\bQ_1\in \real^{n\times n}$ is an orthogonal matrix, and $\bP$ is a permutation matrix. 
Then it is not difficult to find a decomposition satisfying
\begin{equation}\label{equation:orthogonal-complete-qr-or-not}
\begin{bmatrix}
\bR_{11}^\top \\
\bR_{12}^\top
\end{bmatrix}
\triangleq
\bQ_2
\begin{bmatrix}
\bS \\
\bzero 
\end{bmatrix},
\end{equation}
where $\bQ_2$ is an orthogonal matrix, and $\bS$ is a rank-$r$ matrix. The decomposition is reasonable because  the matrix
$\footnotesize
\begin{bmatrix}
\bR_{11}^\top \\
\bR_{12}^\top
\end{bmatrix} \in \real^{p\times r}$ has rank $r$ whose columns stay in a subspace of $\real^p$. Nevertheless, the columns of $\bQ_2$ span the entire space of $\real^p$, where we can assume that the first $r$ columns of $\bQ_2$ span the same space as that of  $\footnotesize\begin{bmatrix}
\bR_{11}^\top \\
\bR_{12}^\top
\end{bmatrix}$. The matrix $\footnotesize\begin{bmatrix}
\bS \\
\bzero 
\end{bmatrix}$ serves to transform $\bQ_2$ into $\footnotesize\begin{bmatrix}
\bR_{11}^\top \\
\bR_{12}^\top
\end{bmatrix}$.
Then,  it follows that 
$
\bQ_1^\top \bX\bP \bQ_2 = 
\begin{bmatrixfoot}
\bS^\top & \bzero \\
\bzero & \bzero 
\end{bmatrixfoot}.
$
Let $\bU \triangleq\bQ_1$, $\bV\triangleq\bQ_2^\top\bP^\top$, and $\bT \triangleq \bS^\top$, we complete the proof.
\end{proof}
We observe that if we consider Equation~\eqref{equation:orthogonal-complete-qr-or-not} as the reduced QR decomposition of $
\footnotesize 
\begin{bmatrix}
\bR_{11}^\top \\
\bR_{12}^\top
\end{bmatrix}$, then the complete orthogonal decomposition reduces to the ULV decomposition; see the next paragraph.

\subsection*{UTV Decomposition}\label{section:utv_decomp}
The \textit{UTV decomposition} goes further from QR and LQ decomposition by factoring the matrix into two orthogonal matrices $\bX=\bU\bT\bV$, where $\bU$ and $\bV$ are orthogonal, whilst $\bT$ is (upper/lower) triangular.
The resulting $\bT$ supports rank estimation. The matrix $\bT$ can be lower triangular which results in the ULV decomposition, or it can be upper triangular which results in the URV decomposition. The UTV framework shares a similar form as the singular value decomposition (SVD, see Theorem~\ref{theorem:reduced_svd_rectangular}) and can be regarded as inexpensive alternative to the SVD.

\begin{theoremHigh}[Full ULV decomposition \citep{hanson1969extensions}]\label{theorem:ulv-decomposition}
Every $n\times p$ matrix $\bX$ with rank $r$ admits the following factorization:
$$
\bX = \bU \begin{bmatrix}
\bL & \bzero \\
\bzero & \bzero 
\end{bmatrix}\bV^\top,
$$
where $\bU\in \real^{n\times n}$ and $\bV\in \real^{p\times p}$ are two orthogonal matrices, and $\bL\in \real^{r\times r}$ is a lower triangular matrix with full rank.
\end{theoremHigh}
The existence of the ULV decomposition is the consequence of the QR and LQ decomposition.
\begin{proof}[of Theorem~\ref{theorem:ulv-decomposition}]
For any rank-$r$ matrix $\bX=[\bx_1, \bx_2, \ldots, \bx_p]$, we can use a column permutation matrix $\bP$ (Definition~\ref{definition:permutation-matrix}) such that the linearly independent columns of $\bX$ appear in the first $r$ columns of $\bX\bP$. Without loss of generality, we assume $\bb_1, \bb_2, \ldots, \bb_r$ are the $r$ linearly independent columns of $\bX$ and 
$$
\bX\bP = [\bb_1, \bb_2, \ldots, \bb_r, \bb_{r+1}, \ldots, \bb_p].
$$
Let $\bZ = [\bb_1, \bb_2, \ldots, \bb_r] \in \real^{n\times r}$. Since any $\bb_i$ ($i\in \{1,2,\ldots, p\}$) is in the column space of $\bZ$, we can find a transformation matrix $\bE\in \real^{r\times (p-r)}$ such that 
$$
[\bb_{r+1}, \bb_{r+2}, \ldots, \bb_p] = \bZ \bE.
$$
That is, 
$$
\bX\bP = [\bb_1, \bb_2, \ldots, \bb_r, \bb_{r+1}, \ldots, \bb_p] = \bZ
[\bI_r , \bE],
$$
where $\bI_r$ is an $r\times r$ identity matrix. Moreover, $\bZ\in \real^{n\times r}$ has full column rank such that its full QR decomposition is given by $\bZ = \bU\begin{bmatrixfoot}
\bR \\
\bzero
\end{bmatrixfoot}$, where $\bR\in \real^{r\times r}$ is an upper triangular matrix with full rank, and $\bU$ is an orthogonal matrix. This implies 
\begin{equation}\label{equation:ulv-smpl}
\bX\bP = \bZ
\begin{bmatrix}
\bI_r & \bE
\end{bmatrix}
=
\bU\begin{bmatrix}
\bR \\
\bzero
\end{bmatrix}
\begin{bmatrix}
\bI_r & \bE
\end{bmatrix}
=
\bU\begin{bmatrix}
\bR & \bR\bE \\
\bzero & \bzero 
\end{bmatrix}.
\end{equation}
Since $\bR$ has full rank, this means 
$[\bR, \bR\bE ]$ also has full rank such that its full LQ decomposition is given by 
$[\bL, \bzero ]\bV_0$, where $\bL\in \real^{r\times r}$ is a lower triangular matrix, and $\bV_0$ is an orthogonal matrix; see Theorem~\ref{theorem:lq-decomposition}. 
Substituting this into Equation~\eqref{equation:ulv-smpl}, we have 
$
\bX = \bU 
\begin{bmatrixfoot}
\bL & \bzero \\
\bzero & \bzero 
\end{bmatrixfoot}
\bV_0 \bP^{-1}.
$
Let $\bV^\top \triangleq \bV_0 \bP^{-1}$, which is a product of two orthogonal matrices and is also an orthogonal matrix. This completes the proof.
\end{proof}
A second way to see the proof of the ULV decomposition is discussed in \citet{lu2021numerical} via the rank-revealing QR (RRQR) decomposition and standard QR decomposition. However, we will not go into further detail here.

\paragrapharrow{Reduced ULV decomposition.} Now suppose the ULV decomposition of matrix $\bX$ is 
$
\bX = \bU 
\footnotesize
\begin{bmatrix}
\bL & \bzero \\
\bzero & \bzero 
\end{bmatrix}
\normalsize
\bV^\top.
$
Let $\bU_1 \triangleq \bU_{:,1:r}$ and $\bV_1 \triangleq \bV_{:, 1:r}$, i.e., $\bU_1$ contains only the first $r$ columns of $\bU$, and $\bV_1$ contains only the first $r$ columns of $\bV$. Then, we still have $\bX = \bU_1 \bL\bV_1^\top$. This is known as the \textit{reduced ULV decomposition}.
The comparison between the reduced and the full ULV decomposition is shown in Figure~\ref{fig:ulv-comparison}, where white entries are zero, and blues entries are not necessarily zero.
\begin{figure}[h]
\centering   
\vspace{-0.35cm}  
\subfigtopskip=2pt  
\subfigbottomskip=2pt 
\subfigcapskip=-5pt 
\subfigure[Reduced ULV decomposition.]{\label{fig:ulvhalf}
\includegraphics[width=0.47\linewidth]{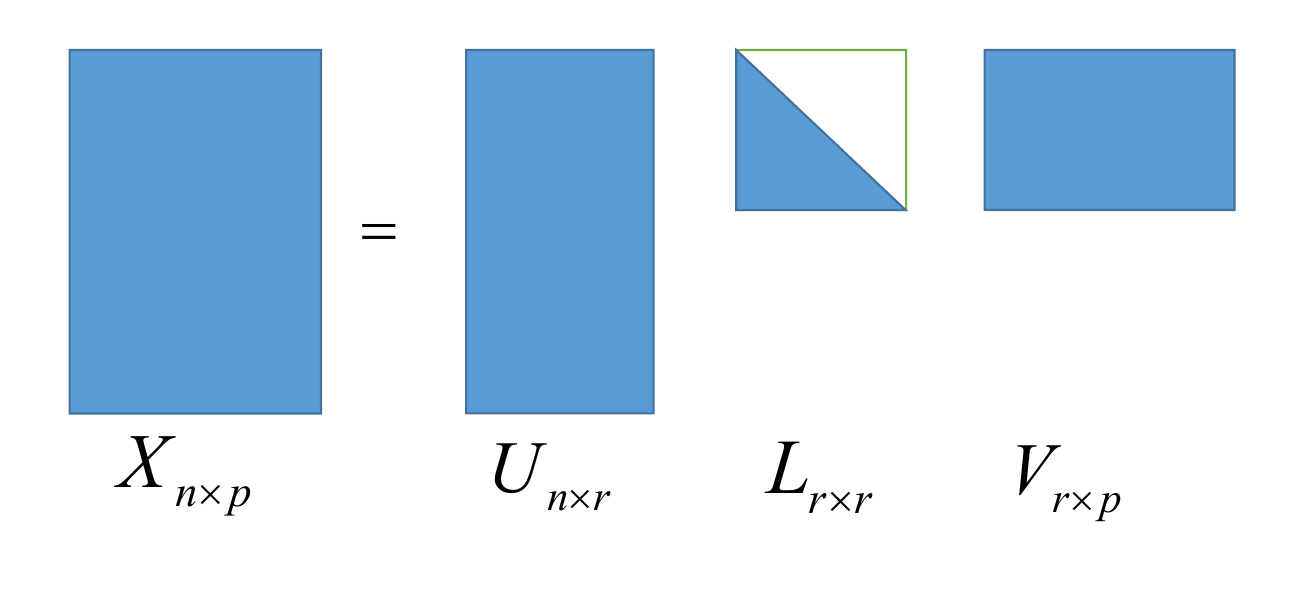}}
\quad 
\subfigure[Full ULV decomposition.]{\label{fig:ulvall}
\includegraphics[width=0.47\linewidth]{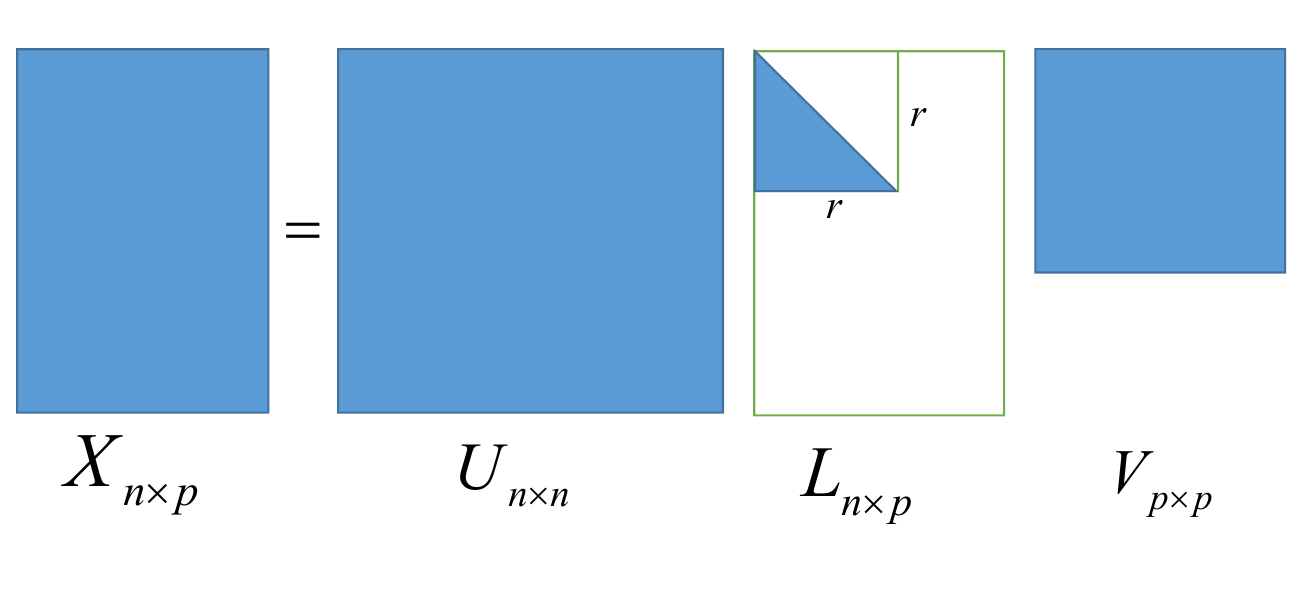}}
\quad
\subfigure[Reduced URV decomposition.]{\label{fig:urvhalf}
\includegraphics[width=0.47\linewidth]{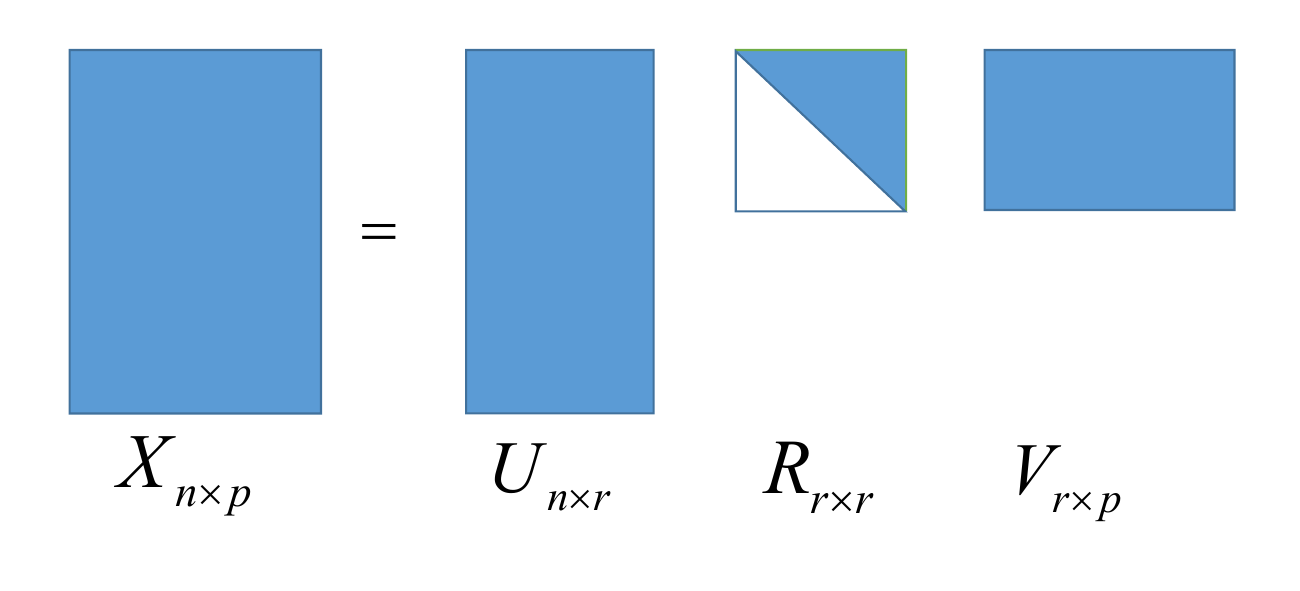}}
\quad 
\subfigure[Full URV decomposition.]{\label{fig:urvall}
\includegraphics[width=0.47\linewidth]{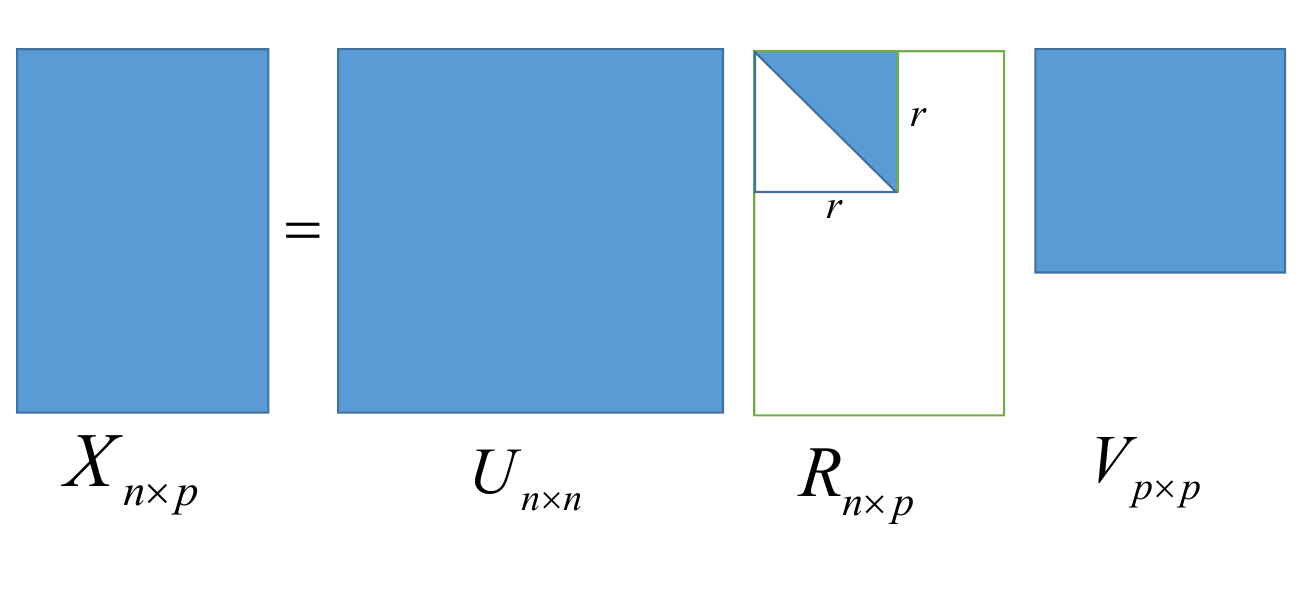}}
\caption{Comparison between the reduced and full ULV, and between the reduced and full URV.}
\label{fig:ulv-comparison}
\end{figure}

\index{Full and reduced}
Similarly, we can also claim the URV decomposition as follows.

\begin{theoremHigh}[Full URV decomposition]\label{theorem:urv-decomposition}
Every $n\times p$ matrix $\bX$ with rank $r$ admits the following decomposition:
$$
\bX = \bU \begin{bmatrix}
\bR & \bzero \\
\bzero & \bzero 
\end{bmatrix}\bV^\top,
$$
where $\bU\in \real^{n\times n}$ and $\bV\in \real^{p\times p}$ are two orthogonal matrices, and $\bR\in \real^{r\times r}$ is an upper triangular matrix with full rank.
\end{theoremHigh}

Again, there is a version of reduced URV decomposition and the difference between the full and reduced URV can be implied from the context, as shown in Figure~\ref{fig:ulv-comparison}. The ULV and URV sometimes are collectively referred to as the UTV decomposition framework \citep{fierro1997low, golub2013matrix}.

We observe that the forms of ULV and URV are very close to the singular value decomposition (SVD). All of the three factor the matrix $\bX$ into two orthogonal matrices. Especially, there exists a set of basis for the four subspaces of $\bX$ in the fundamental theorem of linear algebra via the ULV and the URV. Taking ULV as an example, the first $r$ columns of $\bU$ form an orthonormal basis of $\cspace(\bX)$, and the last $(n-r)$ columns of $\bU$ form an orthonormal basis of $\nspace(\bX^\top)$. The first $r$ columns of $\bV$ form an orthonormal basis for the row space $\cspace(\bX^\top)$, and the last $(p-r)$ columns form an orthonormal basis for $\nspace(\bX)$:
$$
\begin{aligned}
\cspace(\bX) &= \spn\{\bu_1, \bu_2, \ldots, \bu_r\}, &\quad& \nspace(\bX^\top) = \spn\{\bu_{r+1}, \bu_{r+2}, \ldots, \bu_n\},\\
\cspace(\bX^\top) &= \spn\{ \bv_1, \bv_2, \ldots,\bv_r \}, &\quad& \nspace(\bX) = \spn\{\bv_{r+1}, \bv_{r+2}, \ldots, \bv_p\}.
\end{aligned}
$$
The SVD goes further that there is a connection between the two pairs of orthonormal basis, i.e., transforming from column basis into row basis, or left null space basis into right null space basis; see Theorem~\ref{theorem:svd-four-orthonormal-Basis}.
\index{Fundamental spaces}

\paragrapharrow{Computation of URV.}
The CPQR decomposition of a rank-deficient matrix $\bX \in \real^{n\times p}$ with $\rank(\bX)=r$ takes the form 
$
\bX\bP = 
[\bQ_1,\bQ_2]=\bQ\bR\triangleq
\begin{bmatrixfoot}
\bR_{11} & \bR_{12} \\
\bzero   & \bzero 
\end{bmatrixfoot}
$, where $\bR_{11}\in\real^{r\times r}$, with $r < p$, is nonsingular. 
Here, $\bQ_1$ and $\bQ_2$ provide orthogonormal bases for $\cspace(\bX)$ and $\nspace(\bX^\top)$, respectively.  
The elements in $\bR_{12}$ can be annihilated by postmultiplying $\bR$ by a sequence of Householder reflectors
$$
[\bR_{11}, \bR_{12}] \bH_r \ldots \bH_2\bH_1 = \begin{bmatrix}
\widetildebR & \bzero \\
\bzero & \bzero
\end{bmatrix}, \quad \bH_j = \bI - 2 \bu_j \bu_j^\top, \quad j = r, r-1, \ldots, 1,
$$
where $\bu_j$ has nonzero elements only in positions $j, r+1, \ldots, p$. 
The process is equivalent to a QL factorization of the transpose of the triangular factor $\bR$,
$$
\begin{bmatrix}
\bR_{11}^\top & \bzero \\
\bR_{12}^\top & \bzero
\end{bmatrix} = 
\widetildebQ 
\begin{bmatrix}
\widetildebR^\top & \bzero \\
\bzero & \bzero
\end{bmatrix},
$$
and requires \textcolor{black}{$2r^2(p-r)$ flops. }
This obtains the URV decomposition
$$
\bX  = \bQ 
\begin{bmatrix}
\widetildebR & \bzero \\
\bzero & \bzero
\end{bmatrix}\bV^\top\bP^\top, \quad \bV \triangleq \bH_1 \ldots \bH_r.
$$

\index{Rank-deficient}
\index{Rank-deficiency}
\subsection*{LS by UTV}
In Section~\ref{section:application-ls-qr}, we introduced the LS solution using the full QR decomposition for  matrices of full rank.
However, it often happens that the matrix may be rank-deficient. If $\bX$ does not have full column rank, $\bX^\top\bX$ is not invertible. 
Instead of using RRQR decomposition, we can also use the ULV/URV decomposition to compute the least squares solution, as illustrated in the following theorem.\index{Least squares}

\index{Minimum-norm solution}
\begin{theoremHigh}[LS via ULV/URV for rank-deficient matrix]\label{theorem:qr-for-ls-urv}
Let $\bX\in \real^{n\times p}$ with rank $r$ and $n\geq p$. Suppose $\bX=\bU\bT\bV^\top$ is its full ULV/URV decomposition with $\bU\in\real^{n\times n}$ and $\bV^\top\in \real^{p\times p}$ being orthogonal matrix matrices, and
$
\bT \triangleq \begin{bmatrixfoot}
\bT_{1} & \bzero \\
\bzero & \bzero
\end{bmatrixfoot},
$
where $\bT_{1} \in \real^{r\times r}$ is a lower triangular matrix or an upper triangular matrix.
Suppose $\by\in \real^n$ is the response vector.
Then the OLS solution with minimum $\ell_2$ norm to $\bX\bbeta=\by$ is given by 
$$
\widehatbbeta = \bV 
\begin{bmatrix}
\bT_{1}^{-1}\bc_1\\
\bzero	
\end{bmatrix},
$$
where $\bc_1$ contains the first $r$ components of $\bU^\top\by$. 
\end{theoremHigh}
The proof follows immediately since $
\bX^+ = \bV
\begin{bmatrixfoot}
\bT_{1}^{-1} & \bzero \\
\bzero & \bzero
\end{bmatrixfoot} \bU^\top,
$ and $\widehatbbeta=\bX^+\by$ by Theorem~\ref{theorem:unif_ls}.
Alternatively, a more straightforward way to show the minimum-norm solution is provided below.
\begin{proof}[of Theorem~\ref{theorem:qr-for-ls-urv}]
Since $\bX=\bU\bT\bV^\top $ is the full UTV decomposition of $\bX$ and $n\geq p$,
it follows that
$$
\begin{aligned}
\normtwo{\by-\bX\bbeta}^2 
%&= (\by-\bX\bbeta)^\top(\by-\bX\bbeta)
%=(\by-\bX\bbeta)^\top\bU\bU^\top (\by-\bX\bbeta) \\
&\stackrel{\dag}{=}\normtwo{\bU^\top \bX \bbeta-\bU^\top\by}^2
=\normtwo{ \bU^\top\bU\bT\bV^\top  \bbeta-\bU^\top\by}^2\\
&=\normtwo{ \bT\bV^\top  \bbeta-\bU^\top\by}^2 
=\normtwo{\bT_{1}\be_1 - \bc_1}^2+\normtwo{\bc_2}^2,   
\end{aligned}
$$ 
where the equality ($\dag$) follows from the invariance under orthogonal transformation, $\bc_1$ is the first $r$ components of $\bU^\top\by$, and $\bc_2$ is the last $p-r$ components of $\bU^\top\by$; $\be_1$ is the first $r$ components of $\bV^\top \bbeta$, and $\be_2$ is the last $p-r$ components of $\bV^\top \bbeta$:
$$
\bU^\top\by 
\triangleq \begin{bmatrix}
\bc_1 \\
\bc_2 
\end{bmatrix}
\qquad \text{and}\qquad 
\bV^\top \bbeta 
\triangleq \begin{bmatrix}
\be_1 \\
\be_2
\end{bmatrix}.
$$
And the LS solution can be calculated by back/forward substitution of the upper/lower triangular system $\bT_{1}\be_1 = \bc_1$, i.e., $\be_1 = \bT_{1}^{-1}\bc_1$. For $\bbeta$ to have the minimum $\ell_2$ norm, $\be_2$ must be zero. That is,
$
\widehatbbeta = \bV
\begin{bmatrixfoot}
\bT_{1}^{-1}\bc_1\\
\bzero	
\end{bmatrixfoot}.
$
This completes the proof.
\end{proof}

%\paragraph{A word on the minimum $\ell_2$ norm LS solution.} For the least squares problem, the set of all minimizers
%$$
%\sB = \{\bbeta\in \real^p\mid  \normtwo{\by-\bX\bbeta}^2 =\min \}
%$$
%is convex \citep{golub2013matrix}. Given $\bbeta_1, \bbeta_2 \in \sB$ and $\lambda \in [0,1]$, then it follows that
%$$
%\normtwo{\by-\bX(\lambda\bbeta_1 + (1-\lambda)\bbeta_2) } \leq \lambda\normtwo{\by-\bX\bbeta_1} +(1-\lambda)\normtwo{\by-\bX\bbeta_2} = \mathop{\min}_{\bbeta\in \real^p} \normtwo{\by-\bX\bbeta}. 
%$$
%Thus, $\lambda\bbeta_1 + (1-\lambda)\bbeta_2 \in \sB$. In the above proof, if we do not set $\be_2=\bzero$, we will find more least squares solutions. However, the minimum $\ell_2$ norm least squares solution is unique. For the full-rank case discussed in the previous section, the least squares solution is unique and it must have a minimum $\ell_2$ norm. See also \citet{foster2003solving, golub2013matrix} for a more detailed discussion on this topic.

\begin{problemset}
\item Find the gradient descent and greedy descent update rules for the data least squares problem introduced in \eqref{equation:dls}.

\item Prove that the $\bQ$-norm introduced in \eqref{equation:q_norm} is a valid vector norm satisfying Definition~\ref{definition:matrix-norm}.

\item  \label{prob:qnorm_des} Prove that the update step from the $\bQ$-norm in \eqref{equation:qnorm_ugd} is a descent direction satisfying $\innerproduct{\bd_{\text{ugd}}^\toptzero, \nabla f(\btheta^\toptzero)}<0$. 

\item Use the  ``MovieLens 100K" data from MovieLens \citep{harper2015movielens}~\footnote{http://grouplens.org} to evaluate your ALS algorithms. 

% https://math.stackexchange.com/questions/4653146/condition-number-change-in-cholesky-matrix-decomposition
\item \label{problem:cond_pd} Given the Cholesky decomposition of a PD matrix: $\bA=\bL\bD\bL^\top$, show that $\cond(\bA)\geq \cond(\bD)$.

\item Prove the relation stated in  \eqref{equation:qr-orthogonal-equality}.

\item Discuss how to solve the restricted least squares (RLS) problem introduced in Problem~\ref{prob:rls} using either the Cholesky or QR decomposition. 

\item Following the proof of Theorem~\ref{theorem:ulv-decomposition}, prove the existence of the URV decomposition in Theorem~\ref{theorem:urv-decomposition}.

\item \textbf{Indefinite least squares (ILS) \citep{chandrasekaran1998stable}.} For a matrix $ \bX \in \real^{n \times p} $ with $ n \geq p $, and $ \by \in \real^n $, the \textit{indefinite least squares (ILS)} problem is
\begin{equation}\label{equation:ils}
\min_{\bbeta} (\by - \bX\bbeta)^\top \bG (\by - \bX\bbeta), \quad \bX 
= \begin{bmatrix} 
\bX_1 \\ 
\bX_2 
\end{bmatrix}, \quad 
\by = 
\begin{bmatrix} 
\by_1 \\ 
\by_2 
\end{bmatrix},
\end{equation}
where $ \bX_1 \in \real^{n_1 \times p} $, $ \bX_2 \in \real^{n_2 \times p} $, $ n_1 + n_2 = n $, and
$
\bG = \begin{bmatrixfoot} \bI_{n_1} & \bzero \\ \bzero & -\bI_{n_2} \end{bmatrixfoot}
$
is the \textit{signature matrix}. Note that $ \bG^{-1} = \bG $.  
A necessary condition for $ \bbeta $ to be a solution of \eqref{equation:ils} is that the gradient of the objective function vanishes:
$
\bX^\top \bG (\by - \bX\bbeta) = \bzero.
$
This implies that the residual vector $ \be = \by - \bX\bbeta $ is $ \bG $-orthogonal to $ \cspace(\bX) $ (defined in \eqref{equation:phi_norm_invar}). Equivalently, $ \bbeta $ solves the normal equation $ \bX^\top \bG \bX \bbeta = \bX^\top \bG \by $. 
\begin{itemize}
\item Discuss under what conditions the ILS problem has a unique solution.
%The indefinite least squares problem has a unique solution if and only if $ \bX^\top \bG \bX $ is positive definite.  This implies that $ n_1 \geq p $ and that $ \bX_1 $ (and $ \bX $) has full column rank.
\item Compute the reduced QR factorization
$$
\bX = \begin{bmatrix} \bX_1 \\ \bX_2 \end{bmatrix} = \begin{bmatrix} \bQ_1 \\ \bQ_2 \end{bmatrix} \bR = \bQ\bR, \quad \bQ_1 \in \real^{n_1 \times p}, \quad \bQ_2 \in \real^{n_2 \times p},
$$
where $\bQ^\top \bQ = \bQ_1^\top \bQ_1 + \bQ_2^\top \bQ_2 = \bI_p$.
And let $
\bQ_1^\top \bQ_1 - \bQ_2^\top \bQ_2 = \bL \bL^\top
$ be the Cholesky decomposition. Discuss how to solve the ILS problem using these factors. 
% $\bL \bL^\top \bR \bbeta = \bQ^\top \bG \by.$
\end{itemize}

\end{problemset}
\newpage
\chapter{Noise Disturbance and Parameter Estimation}\label{sec:lr-gaussian-noise}
\begingroup
\hypersetup{
linkcolor=structurecolor,
linktoc=page,  % page: only the page will be colored; section, all, none etc
}
\minitoc \newpage
\endgroup
%\section{Linear Model under Gaussian Disturbance}

\index{Random noise}
\index{Gaussian noise}
\section{From Random Noise to Gaussian Noise}
\lettrine{\color{caligraphcolor}I}
In Section~\ref{section:ls_approx1}, we examined the overdetermined system $\by = \bX\bbeta $, where $\bX\in \real^{n\times p}$ denotes a input data matrix of full column rank, $\by\in \real^n$ represents the response vector, and the sample number $n$ exceeds the number of features $p$ such that the columns of $\bX$ are linearly independent. 
The vector $\bbeta$ contains the coefficients (or weights) of the linear model that relates the inputs to the outputs.

In Section~\ref{sec:geometry-noise-disturbance}, we extend this framework by assuming that the observed output $\by$ arises from an ideal function $g(\bX)  $ lying in the column space of $\bX$: $g(\bX)\in \cspace(\bX)$.
Specifically, we model the randomness in observations through the equation:
\begin{equation}
\rvy =g(\bX)+  \bepsilon,
\end{equation}
where $\bepsilon$ represents a noise term. 
This implies that the actual observations $\by$ deviate from the true signal $g(\bX)=\bX\bbeta$ due to additive noise, resulting in the random variable $\rvy$~\footnote{Note again that we use normal fonts of boldface lowercase letters to denote random vectors, and  normal fonts of boldface uppercase letters to denote random matrices. That is, $\rx, \rva, \rmX$ are random scalars, vectors, or matrices; while $x, \ba, \bX$ are scalars, vectors, or matrices. In many cases, the two terms can be used interchangeably; that is, $\rx=x$ denotes a realization of the variable.}.
%In this case, we assume that the observed values $\by$ differ from the true function $g(\bX)=\bX\bbeta$ by additive noise. 
This situation is illustrated in Figure~\ref{fig:ls-geometric2}, which offers a visual interpretation of how noise affects the relationship between the ``true model" and the observed data.
This visual representation provides a comprehensive overview of the problem.

Furthermore, we assume that the noise components $\epsilon_i,\ i\in\{1,2,\ldots,n\}$ are independently and identically distributed (i.i.d.) according to a Gaussian distribution with zero mean and constant variance $\sigma^2$. For each observation $i\in\{1,2,\ldots,n\}$, this leads to the following probabilistic model:
$$
\ry_i = \beta_0 +\beta_1x_{i1} + \beta_2 x_{i2} +\ldots +\beta_{p-1} x_{i,p-1} +\epsilon_i,
$$
where $\beta_0$ serves as the intercept or bias term.

Under the assumption of Gaussian noise, the likelihood function---representing the probability of observing the data given the model parameters---can be derived. This forms the basis for the \textit{maximum likelihood estimator (MLE)}. 
In particular, when the noise follows a normal distribution, the model corresponds to what is commonly known as the \textit{Gauss-Markov linear model}, also referred to as \textit{standard linear regression} or the \textit{Gaussian linear model}.

%The noise assumption together with the mean squared error (MSE) in the model motivate the likelihood, which gives rise to the maximum likelihood estimator (MLE). 
%This is known as the \textit{Gauss-Markov linear model} (or \textit{standard linear regression}, \textit{Gaussian linear model}). 

More concretely, the likelihood function under Gaussian noise is constructed as the product of individual normal density functions. The likelihood of the observed data $\by$ is therefore expressed as:
\begin{equation}\label{equation:likelihood-of-gaussiannoise}
\begin{aligned}
\mathrm{Likelihood}=\mathcalL(\bbeta)
&= p(\by\mid \bX,\bbeta) 
= \prod_{i=1}^n p(y_i\mid \bx_i, \bbeta)\\
&= \prod_{i=1}^n \frac{1}{\sqrt{2\pi\sigma^2}} \exp\left\{-\frac{1}{2\sigma^2}(y_i-\bx_i^\top\bbeta)^2\right\}\\
&= \frac{1}{(2\pi\sigma^2)^{n/2}} \exp\left\{-\frac{1}{2\sigma^2} (\by-\bX\bbeta)^\top(\by-\bX\bbeta)\right\},
\end{aligned}	
\end{equation}
\footnote{For convenience, we slightly abuse the notation by letting $\bx_i$ denote the $i$-th row of the matrix $\bX$. 
In general, however, we use $\bx^{(i)}$ to represent rows of $\bX$ throughout this book.}
which follows a multivariate Gaussian distribution $\normal(\bX\bbeta, \sigma^2\bI)$ (Definition~\ref{definition:multivariate_gaussian}), and
quantifies the plausibility of observing the dataset $\by$ given the model parameters $\bbeta$, input $\bX$, and the assumed noise characteristics.
For computational convenience, it is standard practice to work with the \textit{log-likelihood function}, obtained by taking the natural logarithm of the likelihood expression:
\begin{equation}\label{equation:loglikelihood-of-gaussiannoise}
\begin{aligned}
\text{Log-likelihood}=\ell(\bbeta) = \ln \mathcalL(\bbeta)&= \ln p(\by\mid \bX,\bbeta). \\
%&= \prod_{i=1}^n p(y_i\mid \bx_i, \bbeta).\\
%&= \prod_{i=1}^n \frac{1}{\sqrt{2\pi\sigma^2}} \exp\big(-\frac{1}{2\sigma^2}(y_i-\bx_i^\top\bbeta)^2\big)\\
%&= \frac{1}{(2\pi\sigma^2)^{n/2}} \exp\left(-\frac{1}{2\sigma^2} (\by-\bX\bbeta)^\top(\by-\bX\bbeta)\right),
\end{aligned}	
\end{equation}

\section{Parameter Estimation}

At the beginning of this chapter, we introduced Gaussian noise into the linear model. As a result, we derived the likelihood function and the log-likelihood function, given in \eqref{equation:likelihood-of-gaussiannoise} and \eqref{equation:loglikelihood-of-gaussiannoise}, respectively. 
We have reviewed the concepts of random variables and probability distributions. In some cases, we know that a random vector $\rvx$ follows a particular probability distribution $ p(\bx \mid \btheta) $, but the parameter $\btheta$ of this distribution is unknown.
(e.g., $\bbeta$ in the least squares model \eqref{equation:likelihood-of-gaussiannoise}). 
Similarly, in a Gaussian distribution, the mean $ \mu $, the variance $ \sigma^2 $, or both may be unknown. 
However, if we can obtain observed data from such a random variable, we can use that data to estimate the unknown parameters of the model. 
This observed data is typically referred to as ``\textit{training data}". 
\textit{Model estimation} or \textit{parameter estimation} are methods  of determining the values of parameters within a predefined model structure based on observed data~\footnote{Note that, strictly speaking, \textit{model estimation} refers to the process of determining the values of parameters within a predefined model structure based on observed data; \textit{model fitting} is the process of adjusting a model to match the available data as closely as possible; \textit{model learning} is a broader concept typically used in machine learning, which refers to the entire process of automatically discovering patterns from data and building a predictive or descriptive model.
In this book, we use these terms interchangeably; thus, an estimated mode, a fitted model, and a learned model all refer to the same concept.
}.
In this section, we discuss how to estimate unknown model parameters using observed samples drawn from a given probability distribution.

\subsection{Maximum Likelihood Estimation (MLE)}\label{section:mle_method}

Consider a random vector $ \rvx $, whose probability distribution is $ p(\bx \mid \btheta) $, where $ \btheta $ is an unknown parameter of the distribution. 
We are given a set of observations of this random variable, denoted by $ \mathcalX = \{\bx_1, \bx_2, \ldots, \bx_n\} $, which are assumed to be i.i.d. samples from the same distribution $ p(\bx \mid \btheta) $.

The probability of observing a single sample $ \bx_i $ is $ p(\bx_i\mid \btheta) $. Therefore, the joint probability of observing all $n$ samples is given by:
$ p(\mathcalX\mid \btheta) = p(\bx_1, \bx_2, \ldots, \bx_n\mid \btheta). $
Since the samples are i.i.d., the joint probability can be expressed as the product of individual probabilities:
\begin{equation}
p(\mathcalX\mid  \btheta) = p(\bx_1, \bx_2, \ldots, \bx_n\mid \btheta) = \prod_{i=1}^{n} p(\bx_i\mid \btheta).
\end{equation}
Now, we aim to determine the value of $\btheta$ that best explains the observed data. We assume $\btheta$ lies in a parameter space $ \Theta $. There are many possible values for $\btheta$, so we need a criterion to evaluate and compare them. A natural choice is the joint probability of the observed data under each candidate value of $\btheta$. Intuitively, the most plausible value of $\btheta$ is the one that makes the observed data most probable.

%Now, we need to think about how to obtain the value of $ \btheta $. We assume that $ \btheta $ can take values in a space $ \Theta $. There are many possible values for $ \btheta $, and to determine which one is appropriate, we need an evaluation criterion. 
%By comparing the ``goodness" of different values of $ \btheta $ using this criterion, we can select the optimal value. This evaluation criterion is the joint probability of all observed samples $ p(\mathcalX\mid \btheta) $. Typically, we believe that the most likely event is the one with the highest probability. Since we have already obtained these observed samples, we consider the value of $ \btheta $ that maximizes the probability of these samples to be the optimal one.

%We usually refer to the joint probability of observed samples as the likelihood function, denoted by $ \mathcalL(\btheta; \mathcalX) $ (or simply by $\mathcalL(\btheta)$ if the dataset $\mathcalX$ is known specifically), which is also called the likelihood function.

This joint probability, viewed as a function of $\btheta$, is called the \textit{likelihood function}, denoted by $\mathcalL(\btheta; \mathcalX)$, or simply $\mathcalL(\btheta)$ when the dataset $\mathcalX$ is clear from context:
\begin{equation}
\mathcalL(\btheta; \mathcalX) = p(\mathcalX\mid \btheta) = \prod_{i=1}^{n} p(\bx_i\mid \btheta).
\end{equation}
The optimal value of $ \btheta $ is the value  that maximizes this likelihood function. 
Therefore, the estimate of the parameter $ \btheta $ is given by:
\begin{equation}
\widehatbtheta_{\ML} = \arg\max_{\btheta} \mathcalL(\btheta; \mathcalX) = \arg\max_{\btheta} \prod_{i=1}^{n} p(\bx_i\mid \btheta).
\end{equation}
More formally, we define the maximum likelihood estimator as follows:
\begin{definition}[Maximum likelihood estimator]
Let $\bx_1, \bx_2, \ldots, \bx_n$ be a set of i.i.d. random samples from a distribution $F_{\btheta}$ with density $p(\bx\mid \btheta)$, $\forall\ \btheta \in \Theta$. Then the \textit{maximum likelihood estimator (MLE)} of $\btheta$ is the value  $\widehatbtheta_{\ML}$ such that 
\begin{equation}
\mathcalL(\btheta; \mathcalX) \leq \mathcalL(\widehatbtheta_{\ML}), \ \forall\btheta \in \Theta, \nonumber
\end{equation}
where $\mathcalL(\btheta; \mathcalX) = \prod_{i=1}^{n} p(\bx_i\mid  \btheta)$ is the \textit{likelihood function} for the i.i.d. collection. That is, the MLE of $\btheta$ can be obtained by 
$$
\widehatbtheta_{\ML} = \mathop{\arg\max}_{\btheta\in \Theta}	\mathcalL(\btheta; \mathcalX). 
$$
\end{definition}

Notice that the likelihood function $\mathcalL(\btheta; \mathcalX)$ is a \textit{random function}, since it depends on the random samples $\bx_1, \bx_2, \ldots, \bx_n$. And the meaning of the likelihood function is the probability of these specific observed samples $\{\bx_1, \bx_2, \ldots, \bx_n\}$ when the parameter is taken to be equal to $\btheta$ rather than the probability of the parameter $\btheta$. In other words, it is the joint density of the sample, but viewed as a function of $\btheta$.

\begin{remark}[\textbf{Estimation method} vs \textbf{estimator} vs \textbf{estimate}]
Note that an \textit{estimation method} is a general algorithm to produce the estimator. \textit{An estimate} is the specific value that \textit{an estimator} takes when observing the specific value, i.e., an estimator is a random variable and the realization of this random variable is called an estimate.
\end{remark}

A key observation about the likelihood function is that it is defined as the product of individual probabilities $ p(\bx_i\mid \btheta) $. 
Since each term in this product lies between 0 and 1, multiplying many such terms results in a very small number, often too small to be accurately represented by computer systems due to floating-point precision limits.
To avoid numerical underflow and simplify computation, especially during optimization, we typically work with the log-likelihood function, denoted by $ \ell(\btheta; \mathcalX)$ or simply $ \ell(\btheta)  $ when the data $\mathcalX$ is understood:
\begin{equation}
\ell(\btheta; \mathcalX) = \ln \mathcalL(\btheta; \mathcalX).
\end{equation}
Maximizing the log-likelihood function $ \ell(\btheta; \mathcalX) $ to obtain $ \widehatbtheta_{\ML} $ is mathematically equivalent to  maximizing the original likelihood function $ \mathcalL(\btheta; \mathcalX) $, since the logarithm  is a monotone increasing function:
\begin{equation}\label{equation:mle_loglik}
\widehatbtheta_{\ML} = \arg\max_{\btheta} \ell(\btheta; \mathcalX)
= \arg\max_{\btheta} \ln \prod_{i=1}^{n} p(\bx_i\mid \btheta)
= \arg\max_{\btheta} \sum_{i=1}^{n} \ln p(\bx_i\mid \btheta).
\end{equation}

\paragrapharrow{Estimation Methods For MLE.}
To compute the maximum likelihood estimate, we need to maximize the (log-)likelihood function. There are three primary methods commonly used for this purpose:
\begin{enumerate}[(i)]
\item \textit{Analytic method}, also known as the direct solution method. 
This approach involves finding critical points of the log-likelihood function by setting its gradient (i.e., vector of partial derivatives with respect to $\btheta$) equal to zero:
\begin{equation}\label{equation:mle_stat}
\nabla_{\btheta} \ell(\btheta; \mathcalX)=\frac{\partial \ell(\btheta; \mathcalX) }{\partial \btheta} = \bzero.
\end{equation}
The solutions to this equation are called \textit{stationary points}, which may correspond to local maxima, minima, or saddle points. Therefore, Equation~\eqref{equation:mle_stat} provides a necessary but not sufficient condition for a maximum.
The nature of these points (maximum or minimum) can be determined by the second derivative at the stationary points (see, for example, \citet{lu2025practical}). 
Not all problems can yield an analytic solution, and often it is not directly solvable. 
When a unique maximum of the likelihood function exists, we refer to it as  \textit{the} maximum likelihood estimator $\widehatbtheta = \mathop{\arg\max}_{\btheta\in \Theta}	\ell(\btheta; \mathcalX)$. 
If the likelihood function is twice differentiable, this can be done by verifying the   second derivative (Hessian matrix, see Problem~\ref{problem:pos_hessian}, or \citet{lu2025practical}):
$$
-\nabla^2_{\btheta} \ell(\btheta; \mathcalX)|_{\btheta=\widehatbtheta} \succ \bzero.
$$

\item \textit{Grid search method}. 
When the parameter space $\Theta$ is low-dimensional and bounded, we can discretize it into a grid of candidate values and evaluate the likelihood (or log-likelihood) at each point. The value that yields the highest likelihood is taken as the estimate. Although conceptually simple and guaranteed to find the global maximum (given a sufficiently fine grid), this method becomes computationally infeasible as the number of parameters increases, due to the exponential growth of the grid size (a problem known as the ``curse of dimensionality").

\item \textit{Numerical method}.
These are the most widely used techniques in practice. They begin with an initial guess $\btheta^\topone$, and iteratively update the parameter estimate using information from the gradient (and possibly the Hessian) of the log-likelihood function.  
Gradient descent (ascent), Newton-Raphson, mirror descent (ascent), all fall into this category \citep{lu2025practical}.
These iterative procedures are particularly effective in high-dimensional settings where analytic solutions are unavailable or grid search is impractical
\end{enumerate}

In this chapter, we will primarily focus on deriving and applying the analytic method for computing maximum likelihood estimators.

\index{Bayesian estimation}
\index{Point estimation}
\index{Maximum likelihood estimation}
\subsection{Bayesian Estimation}\label{section:bayes_esti}

Another related estimation method is called the \textit{maximum a posteriori (MAP) estimation} (see Section~\ref{equation:map_esti}).
Before discussing the details of MAP estimation, we first introduce the fundamental components of another widely used parameter estimation approach: the \textit{Bayesian estimation} or the \textit{Bayesian approach}; its application to linear models will be discussed  in detail in Chapter~\ref{chapter:bayes_app_mle}. 
The foundational idea of Bayesian estimation is attributed to \textit{Thomas Bayes}, who developed the concept but passed away before publishing it. Fortunately, his friend \textit{Richard Price} continued his work and published it in 1764. The same principle was later independently rediscovered by \textit{Laplace} at the end of the 18-th century.
In this section, we present the basic concepts of the Bayesian approach.

In MLE, the parameter $\btheta$ is treated as a fixed numerical quantity,  and only the variable $\mathcalX$ is considered random. 
The probability distribution of the random variable $\mathcalX$ is described by a parametric distribution $p(\mathcalX \mid \btheta)$, where observed samples are used to estimate the unknown parameter $\btheta$. This leads to an estimate $\widehatbtheta$, which is then substituted back into the conditional probability function $p(\mathcalX\mid {\btheta})$, yielding an estimated distribution for $\mathcalX$. This distribution can subsequently be used for predicting new samples: $p(\mathcalX=\bx_{\new} \mid \widehatbtheta)$.

\subsection*{Prior Distribution and Joint Probability}
A key distinction in the Bayesian framework is that unknown parameters are treated as random variables. This aligns with the Bayesian perspective on uncertainty, which treats all uncertain quantities probabilistically by modeling them as random variables and applying the laws of probability to reason about them. Unlike MLE, which focuses on finding the most likely value of a parameter, Bayesian methods aim to account for all possible values through integration over the full parameter space.

More formally, in the Bayesian framework, the model parameter $\btheta$ is treated as a random variable. A sample is generated jointly by $\btheta$ and $\mathcalX$. Let the probability distribution of $\btheta$ be  $p(\btheta)$, known as the \textit{prior distribution}. The random variable $\mathcalX$ depends on $\btheta$, and thus its distribution is given by the conditional probability $p(\mathcalX \mid \btheta)$. Together, these define a joint probability distribution from the chain rule of probability:
$$
p(\mathcalX, \btheta) = p(\btheta) p(\mathcalX\mid \btheta).
$$
Here, the conditional probability distribution $p(\mathcalX\mid \btheta)$ is the probability distribution of variable $\mathcalX$, which is known. 
While the prior $p(\btheta)$ reflects our initial beliefs or knowledge about the parameter before observing any data. Our ultimate goal is to determine the marginal distribution of $\mathcalX$, which can then be used for prediction.
There are two main approaches to achieve this:
\begin{enumerate}[(i)]
\item Find an estimate $\widehatbtheta$ of $\btheta$, and then obtain the conditional probability distribution $p(\mathcalX\mid \widehatbtheta)$. Use the conditional probability distribution $p(\mathcalX\mid \widehatbtheta)$ as the probability distribution of $\mathcalX$ for subsequent prediction and analysis. Maximum likelihood estimation falls into this category.

\item Use the joint distribution $p(\mathcalX, \btheta)$ to compute the marginal distribution of $\mathcalX$ by integrating out the parameter $\btheta$:
$$
p(\mathcalX) = \int p(\mathcalX, \btheta) \, d\btheta = \int p(\btheta)\,p(\mathcalX \mid \btheta) \, d\btheta.
$$
In this case, the marginal distribution $p(\mathcalX)$ directly accounts for all possible values of $\btheta$, weighted by their prior probabilities. This is the core idea behind Bayesian estimation.
\end{enumerate}

The challenge with the second approach lies in specifying the prior distribution $p(\btheta)$. Without knowledge of the prior, it is not possible to compute the marginal distribution$p(\mathcalX)$. Bayesian estimation belongs to the second method, utilizing Bayes' theorem to update our beliefs about $\btheta$ in light of observed data, thereby leading to a full \textit{posterior distribution} over the parameter.

\index{Bayes' theorem}
\subsection*{Bayes' Theorem}
We know that the variables $\btheta$ and $\mathcalX$ form a joint probability $p(\btheta, \mathcalX)$, and these two variables are not  independent of each other. 
In this context, the variable $\btheta$ influences the variable $\mathcalX$, and there exists a ``cause-effect" relationship between them: $\btheta$ is the ``cause" or \textit{prior}, and $\mathcalX$ is the ``effect." 
Using the chain rule of probability, the joint distribution $p(\mathcalX, \btheta)$ can be expressed as a product of conditional probabilities. Importantly, the application of the chain rule depends on the dependence (or independence) relationships between the variables, not on the order in which they appear. Therefore, the joint probability $p(\mathcalX, \btheta)$ can be decomposed in two equivalent ways:
$$ 
p(\mathcalX, \btheta) = p(\btheta) p(\mathcalX\mid \btheta) = p(\mathcalX) p(\btheta\mid \mathcalX) .
$$
Rearranging terms yields that
\begin{equation}\label{equation:bayes_firs}
p(\btheta\mid \mathcalX) = \frac{p(\btheta) p(\mathcalX\mid \btheta)}{p(\mathcalX)}.
\end{equation}
This equation is known as \textit{Bayes' theorem}. To define it formally:
\index{Bayes' theorem}
Bayesian modeling and statistics are fundamentally driven by  Bayes' theorem. Formally, the theorem is expressed as follows.
\begin{theoremHigh}[Bayes' theorem]\label{theorem:baye_theo_mle}
Let $\sS$ be a sample space and let $B_1, B_2, \ldots, B_K$ be a partition of $\sS$ such that (1). $\cup_k B_k=\sS$ and (2). $B_i \cap B_j=\varnothing$ for all $i\neq j$. 
Let further $A$ be any event. Then it follows that 
$$
P(B_k \mid A) = \frac{P(A \mid B_k)P(B_k)}{P(A)} = \frac{P(A\mid B_k)P(B_k)}{\sum_{i=1}^{K}P(A\mid B_i)P(B_i)}.
$$
\end{theoremHigh}
The core of Bayes' theorem is the following transformation:
$$ 
p(\text{cause}\mid \text{effect}) 
= \frac{p(\text{cause}) p(\text{effect}\mid \text{cause})}{p(\text{effect})}.
$$
In many practical situations, we observe the ``effect"---that is, we have observed values of the variable $\mathcalX$---but we do not know what caused this effect, i.e., the value of the variable $\btheta$ is unknown. In such cases, we can use Bayes' theorem to infer the likely cause, a process commonly referred to as \textit{Bayesian inference} or \textit{Bayesian estimation}. 

\subsection*{Bayesian Inference}
We consider variable $\btheta$ as the ``cause" variable and variable $\mathcalX$ as the ``effect" variable. 
The observed value of $\mathcalX$ represents the observed result. We express the relationship between the observed sample $\mathcalX$ and the variable $\btheta$ using Bayes' theorem:
\begin{equation}\label{equation:bayes_seco}
p(\btheta\mid \mathcalX) = \frac{p(\mathcalX\mid \btheta) p(\btheta)}{p(\mathcalX)},
\end{equation}
where
$p(\btheta\mid \mathcalX)$ represents the \textit{posterior probability distribution} of the cause variable $\btheta$ based on the effect $\mathcalX$. 
The term ``posterior" indicates that this distribution is updated after observing the data (i.e., after gaining experience from the effect).
$p(\btheta)$ represents the \textit{prior probability distribution},  which reflects our knowledge or assumptions about $\btheta$ before observing the data. In practical applications, the prior can be chosen based on domain knowledge or previous experience. If no prior information is available, it is common to assume a uniform distribution for $p(\btheta)$. 
$ p(\mathcalX\mid \btheta) $ is the likelihood function, which represents the probability of observing the sample $\mathcalX$ given the parameter $\btheta$. If the observed samples are independent and identically distributed, this term can be written as:
$$
p(\mathcalX\mid \btheta) = p(\{\bx_1, \bx_2, \ldots, \bx_n\}\mid \btheta)
= \prod_{i=1}^{n} p(\bx_i\mid \btheta).
$$
$ p(\mathcalX) $ is referred to as the \textit{evidence} or \textit{marginal likelihood}. It represents the total probability of observing the data $\mathcalX$, regardless of the value of $\btheta$. It acts as a normalizing constant that ensures the posterior distribution integrates to 1 (i.e., it is a valid probability distribution). Mathematically, it is defined as:
$$
p(\mathcalX) = \int p(\mathcalX\mid \btheta) p(\btheta) d\btheta.
$$
In other words, once the sample set $\mathcalX$ is determined, the value of $ p(\mathcalX) $  remains constant.

In summary, the Bayesian inference formula in \eqref{equation:bayes_seco} can be expressed in the following form:
\begin{equation}\label{equation:bayes_thir}
\text{Posterior probability} = \frac{\text{Likelihood} \times \text{Prior}}{\text{Evidence}}
\propto \text{Likelihood} \times \text{Prior},
\end{equation}
where ``$\propto$" means ``proportional to" (see Problem~\ref{problem:proptionto}).
%Using Bayesian inference, we can compute the posterior distribution $ p(\btheta\mid \mathcalX) $. Once we have this posterior, we can substitute it back into the joint distribution:
%We can use Bayesian inference to find the posterior probability distribution of the parameter variable $\btheta$, $ p(\btheta\mid \mathcalX) $, and then substitute $ p(\btheta\mid \mathcalX) $ into \eqref{equation:bayes_seco}. This way, we determine the joint probability distribution of variables $\mathcalX$ and $\btheta$, and thus obtain the marginal probability distribution of $\mathcalX$:
%\begin{equation}\label{equation:bayes_fourth}
%p(\mathcalX) = \int p(\mathcalX, \btheta) d\btheta
%= \int p(\btheta\mid \mathcalX) p(\mathcalX\mid \btheta) d\btheta.
%\end{equation}
However, using \eqref{equation:bayes_thir} to infer $ p(\btheta\mid \mathcalX) $ still has two difficulties:
\begin{enumerate}[(i)]
\item Choosing an appropriate prior distribution $p(\btheta)$. This should ideally reflect prior knowledge about the parameters, but such information is often unavailable.
\item Evaluating the denominator $p(\mathcalX)$, which involves integrating over all possible values of $\btheta$. The complexity of this integration depends heavily on the form of the prior $p(\btheta)$.
\end{enumerate}

Note that, for brevity, we previously omitted the hyper-parameter from the prior distribution.
When including a hyper-parameter in the prior, Equation~\eqref{equation:bayes_seco} becomes:
\begin{equation}\label{equation:baye_hyper}
\begin{aligned}
p(\btheta \mid  \mathcalX, \balpha) 
&= \frac{p(\mathcalX \mid  \btheta ) p(\btheta \mid  \balpha)}{p(\mathcalX \mid \balpha)} 
= \frac{p(\mathcalX \mid  \btheta ) p(\btheta \mid  \balpha)}{\int_{\btheta}  p(\mathcalX, \btheta \mid \balpha) }  \\
&= \frac{p(\mathcalX \mid \btheta ) p(\btheta \mid \balpha)}{\int_{\btheta}  p(\mathcalX \mid  \btheta ) p(\btheta \mid  \balpha) }  
\propto p(\mathcalX \mid \btheta ) p(\btheta \mid \balpha).
\end{aligned}
\end{equation}
where the marginal likelihood or evidence becomes $p(\mathcalX\mid \balpha)$.
This form will prove useful when discussing \textit{Occam's Razor} and \textit{Occam factor} in Section~\ref{section:occam_razor}.
For the remainder of this section, we will continue using the version without a hyper-parameter.

Theoretically, the prior distribution should encode any existing or priori knowledge (e.g., the parameter is sparse or dense) about the parameters. For example, a prior for a system reducing over-clustering might assign a higher probability to a larger cluster than to a small cluster \citep{lu2021survey}. 
However, in many cases, little or no prior information is available. In such situations, we use a special type of prior called a \textit{noninformative prior}. The purpose of a noninformative prior is to have minimal influence on the posterior distribution \citep{gelman2013bayesian}, allowing the data to ``speak for itself."

Another commonly used approach is the \textit{conjugate prior}. In Bayesian inference, a conjugate prior is a prior distribution such that the resulting posterior distribution belongs to the same family of distributions as the prior. This simplifies computation and interpretation.
Conjugate priors are widely used due to their mathematical convenience. Further details  can be found in Chapter~\ref{chapter:bayes_app_mle}.

The elegance of Bayes' theorem becomes apparent as it distinguishes inference from modeling. The model, encompassing the prior distribution and the likelihood, fully dictates the posterior distribution, leaving the computation of the inference as the only remaining step.
More generally, the Bayesian approach---in a nutshell---is to assume a prior distribution for any unknowns ($\btheta$ in our case), and then just follow the rules of probability to answer any questions of interest. For example, when we find the parameter based on the maximum posterior probability of $\btheta$, we turn to the \textit{maximum a posteriori (MAP)} estimation; see the next section.

\subsection*{Prediction}
The posterior distribution alllows us to compute the probability density at a new coming data point $\bx_{\new}$, called the \textit{posterior predictive distribution}, by averaging over both the uncertainty in the model and in the parameters:
$$
p(\bx_{\new} \mid \mathcalX) = \int p(\bx_{\new} \mid \btheta) p(\btheta \mid \mathcalX)\, d\btheta.
$$
If the problem  follows from a generative process $\rvy\sim p(\by\mid \bx,\btheta)$, e.g., $\ry\sim \bbeta^\top\bx_{\new} +\epsilon$ in the Gauss-Markov linear model. 
Then the predictive distribution is 
$$
p(\by^\prime\mid\bx_{\new},\mathcalX, \mathcalY)=\int p(\by^\prime \mid \bx_{\new}, \btheta)p(\btheta \mid \mathcalX,\mathcalY)\,d\btheta.
$$
The posterior predictive distribution can be employed to design test statistics of interest and then compare the posterior predictive distributions to the test statistics of observed values so as to determine the best model among several candidates. This process is known as \textit{model checking or selection}; see Chapter~\ref{chapter:model_eva_sel}.

\index{Model checking}
\index{Model selection}

\index{Maximum a posteriori estimation}
\subsection{Maximum A Posteriori (MAP) Estimation}\label{equation:map_esti}

One major challenge in Bayesian estimation is computing the marginal likelihood $ p(\mathcalX) $, which requires integrating over the entire parameter space. This integration can be computationally expensive or even analytically intractable.

However, if our goal is only to make predictions and not to fully characterize the posterior distribution, we may instead use a point estimate of the parameters. This approach resembles maximum likelihood estimation, but incorporates prior information.

%One of the major challenges in Bayesian estimation is calculating $ p(\mathcalX) $. Calculating $ p(\mathcalX) $ requires integration over the parameter space, and the cost of this integration operation is often very high, sometimes even impossible to compute. If we only need to predict new samples  and do not need to explore the parameter variables extensively, then we do not need to obtain the complete posterior distribution but only a point estimate of the parameters, similar to maximum likelihood estimation.

%Bayesian estimation involves calculating the expected value of the posterior probability distribution as the estimated value of the parameters. However, calculating the expected value requires complete computation of the probability distribution. Besides the expected value, we can also use the parameter value that maximizes the posterior probability distribution as the estimated value of the parameters, known as \textit{maximum a posteriori (MAP) estimation}.

In Bayesian estimation, one common method is to compute the expected value of the posterior distribution as the estimate of the parameters. However, calculating this expectation requires full knowledge of the posterior distribution.
Alternatively, we can use the value of $\btheta$ that maximizes the posterior distribution as our estimate. This is known as \textit{maximum a posteriori (MAP)} estimation:
$$
\widehatbtheta_{\MAP} = \arg\max_{\btheta} p(\btheta\mid \mathcalX).
$$
Formally, the MAP estimator is defined as follows.
\begin{definition}[Maximum a posterior estimator]\label{definition:map_esti}
The maximum a posteriori (MAP) estimate is the value of the parameter that maximizes the posterior distribution.
It balances information from the prior distribution and information from the likelihood. 
The influence of the prior is stronger when the likelihood provides less information, and vice versa.
\end{definition}

Recall from \eqref{equation:bayes_seco} that the denominator $ p(\mathcalX) $ of the posterior probability  is a constant with respect to $\btheta$. 
Therefore, the posterior probability is proportional to the product of the prior and the likelihood:
$$
\text{Posterior Probability} = \frac{\text{Likelihood} \times \text{Prior}}{\text{Evidence}}
\propto \text{Likelihood} \times \text{Prior}.
$$
When performing MAP estimation, it is not necessary to compute the exact form of the posterior distribution, since we are only interested in its maximum. Thus, maximizing the numerator alone suffices:
%it is  not necessary to first calculate the specific form of the posterior probability distribution $ p(\btheta\mid \mathcalX) $, because the posterior probability is proportional to the numerator $ p(\btheta\mid \mathcalX) \propto \mathcalL(\btheta; \mathcalX) p(\btheta) $. Therefore, when maximizing to solve the posterior probability distribution, only the numerator needs to be maximized.
$$
\begin{aligned}
\widehatbtheta_{\MAP} &= \arg\max_{\btheta} p(\btheta\mid \mathcalX)
\equiv \arg\max_{\btheta} \text{Likelihood} \times \text{Prior} \\
&= \arg\max_{\btheta} \mathcalL(\btheta; \mathcalX) p(\btheta)
= \arg\max_{\btheta} \ln \mathcalL(\btheta; \mathcalX) p(\btheta) \\
&= \arg\max_{\btheta} \underbrace{\ln \mathcalL(\btheta; \mathcalX)}_{\text{Log-likelihood}} + \underbrace{\ln p(\btheta)}_{\text{Log Prior}}.
\end{aligned}
$$
Compared to the MLE in \eqref{equation:mle_loglik}, we see that MAP estimation also maximizes the likelihood, but it includes an additional term---the log prior---which regularizes the estimate based on our prior beliefs about $\btheta$. In many ways, MAP estimation resembles maximum likelihood estimation, but it incorporates prior information, making it more robust in cases where data is limited.

In fact, incorporating a prior in MAP estimation is equivalent to adding a regularization term to the loss function. Specifically, introducing a Laplace prior for the parameters is equivalent to applying $\ell_1$ regularization, while using a Gaussian prior corresponds to $\ell_2$ regularization.

\index{Laplace approximation}
\subsection{Laplace Approximation}\label{section:laplace_approx}
We previously noted that the posterior distribution can be used to answer any questions of interest, including obtaining the MAP estimate:
$$
\widehatbtheta_{\text{MAP}} 
= \argmax_{\btheta\in\Theta} p(\btheta \mid  \mathcalX) 
= \argmax_{\btheta\in\Theta} p(\mathcalX \mid \btheta ) p(\btheta ).
$$
The \textit{Laplace approximation} involves approximating the posterior using a Gaussian distribution centered at the mode of the posterior (i.e., the MAP estimate $\widehatbtheta_{\text{MAP}} $), therefore approximating the posterior distribution of a model's parameters when the exact form of that distribution is intractable or computationally expensive to calculate  \citep{kass1995bayes, mackay1998choice, friston2007variational}. 
Define the logarithm of the posterior distribution as 
$$
\ell(\btheta) =\ln p(\mathcalX \mid \btheta ) p(\btheta )
=\ln p(\mathcalX \mid \btheta ) + \ln p(\btheta ) =\ln p(\btheta \mid  \mathcalX) + \mathcalC,
$$
where again $\ln$ is the natural logarithm (to base $e$),  and $\mathcalC$ represents a constant w.r.t. $\btheta$.
According to the quadratic approximation theorem (Theorem~\ref{theorem:quad_app_theo}), assuming that the parameter space $\Theta$ is an open set (the gradient of the MAP has vanished gradient), we have 
$$
\begin{aligned}
	\ell(\btheta) &\approx 
	\ell(\widehatbtheta) +\nabla \ell(\widehatbtheta )^\top (\btheta - \widehatbtheta )
	+\frac{1}{2} (\btheta - \widehatbtheta )^\top \nabla^2 \ell(\widehatbtheta ) (\btheta - \widehatbtheta )\\
	&=\ell(\widehatbtheta ) 
	+\frac{1}{2} (\btheta - \widehatbtheta )^\top \nabla^2 \ell(\widehatbtheta ) (\btheta - \widehatbtheta ),
\end{aligned}
$$
where we let $\widehatbtheta\triangleq \widehatbtheta_{\text{MAP}} $ for brevity.
Therefore, the log marginal likelihood can be obtained by 
$$
\begin{aligned}
	\begin{aligned}
		\ln p(\mathcalX ) &= \ln \int p(\mathcalX \mid \btheta) p(\btheta ) d\btheta 
		=\ln  \int \exp\{\ell(\btheta) \}d\btheta\\
		&\approx \ln p(\mathcalX \mid \widehatbtheta ) + \ln p(\widehatbtheta ) + \frac{p}{2} \ln (2\pi) - \frac{1}{2}\ln \abs{\nabla^2 \ell(\widehatbtheta )},
	\end{aligned}
\end{aligned}
$$
where  the last approximation comes from the definition of the multivariate Gaussian distribution, and $p$ is the dimensionality  of the parameter vector: $\btheta\in\real^p$.
Therefore, the Laplace approximation of the marginal likelihood becomes 
\begin{equation}\label{equation:laplace_approx}
p(\mathcalX )_{\text{Lap}} = 
\underbrace{p(\mathcalX \mid \widehatbtheta )}_{\text{data likelihood under MAP}} \,
\underbrace{p(\widehatbtheta )}_{\text{penalty from prior}} 
\underbrace{\abs{2\pi (\nabla^2 \ell(\widehatbtheta ))^{-1}}}_{\text{local curvature}}.
\end{equation}
Thus, the Laplace approximation contains three terms: the data likelihood under the MAP estimate, a penalty term from the prior, and a volume term due to the local curvature.

While the Laplace approximation can be useful and computationally efficient, it has several drawbacks. 
\begin{itemize}
	\item \textbf{Gaussian assumption}: Firstly, the Laplace approximation assumes that the posterior distribution is approximately Gaussian, which might not be a good assumption, especially for multimodal posteriors or posteriors that have heavy tails or skewed distributions. This can lead to inaccurate estimates of the uncertainty around the parameters. 
	The Gaussian approximation is also ill-suited for parameters that are bounded, constrained, or positive, such as mixing proportions or precisions, because it assigns nonzero probability mass to values outside the valid parameter range. While this issue can often be mitigated by reparameterizing the variables (see, for example, \citet{mackay1998choice}), there remains an undesirable aspect: in the non-asymptotic regime, the approximation lacks invariance to reparameterization.
	
	\item \textbf{Mode dependence}: The approximation relies heavily on the location of the mode of the posterior distribution. If the mode is not well-defined or if there are multiple modes, the Laplace approximation may perform poorly.
	
	\item \textbf{Curvature assumption}: The Laplace approximation assumes that the curvature of the posterior distribution around the mode is constant, which is often not the case for complex models. This can lead to poor performance when the posterior distribution has significant curvature changes over its support.
	
	\item \textbf{Computation of Hessian}: Computing the Hessian matrix, which is required to determine the variance of the Gaussian approximation, can be computationally expensive and unstable, particularly for models with many parameters or non-smooth likelihood functions. The computation of the volume term, which depends on the determinant of the Hessian matrix ($|\nabla^2 \mathcalL(\widehatbtheta )|$), poses another challenge. Calculating the derivatives within the Hessian requires  $\mathcalO(np^2)$ operations, followed by $\mathcalO(p^3)$ operations to find the determinant, making it computationally intensive for high-dimensional problems. To simplify this process, approximations often ignore off-diagonal elements or assume a block-diagonal structure for the Hessian, effectively disregarding interdependencies among parameters. 
	
	\item \textbf{Sensitivity to priors}: The Laplace approximation can be sensitive to the choice of prior, especially when the prior is not weakly informative. In such cases, the approximation might not accurately reflect the shape of the posterior distribution.
	
	\item \textbf{Dimensionality issues}: As the number of parameters increases, the Laplace approximation becomes less reliable due to the curse of dimensionality, where the volume of the parameter space grows exponentially and the Gaussian approximation becomes increasingly poor.

\end{itemize}
Despite these limitations, the Laplace approximation can still be a valuable tool, especially when used alongside other methods like Markov chain Monte Carlo (MCMC) sampling or variational inference, which can provide more accurate representations of the posterior distribution \citep{hoff2009first, lu2023bayesian}.

\index{Mean squared error}
\subsection{Mean Squared Error (MSE) of an Estimator}
Let $\rvx_1, \rvx_2, \ldots, \rvx_n$ be $n$ i.i.d. random vector variables, i.e., a random sample drawn from the distribution $p(\bx \mid \btheta)$, where $\btheta$ is unknown. An \textit{estimator} of $\btheta$ is any function of (only) these $n$ random variables; in other words, it is a statistic $\widehatbtheta = T(\rvx_1,\rvx_2, \ldots, \rvx_n)$. 
As discussed in previous sections, there are several methods available for obtaining an estimator of $\btheta$, such as maximum likelihood estimation and Bayesian methods.

A key challenge arises when multiple estimation methods can be applied to the same problem, leading us to choose among different estimators. In some cases, different methods may yield the same estimator (for example, in least squares estimation under certain conditions with MLE), simplifying the comparison. However, in many situations, each method results in a distinct estimator. Therefore, we need criteria to evaluate and compare their performance.
There are various measures used to assess the quality of an estimator. Some are designed for small samples, while others describe the behavior of an estimator as the sample size grows---these are known as asymptotic properties of estimators.

To evaluate how close an estimator $\widehatbtheta$ is to the true value $\btheta$, we often consider the deviation  $\normonebig{\widehatbtheta - \btheta}$ or, for computational convenience, the squared error  $\normtwobig{\widehatbtheta - \btheta}^2$. Since $\widehatbtheta$ is itself a random variable (depending on the random sample), we must take expectation over the sampling distribution to assess its overall accuracy. This leads to the following definition:
\begin{definition}[Mean squared error of an estimator]
Let $\rvx_1, \rvx_2, \ldots, \rvx_n$ be $n$ i.i.d. random vector variables.
The \textit{mean squared error (MSE)} or the \textit{risk function} of an estimator $\widehatbtheta= T(\rvx_1,\rvx_2, \ldots, \rvx_n)$ of a parameter $\btheta$ is the function of $\btheta$ defined by $\Exp[\normtwobig{\widehatbtheta - \btheta}^2]$ (the expectation is taken with respect to the random variables $\rvx_1,\rvx_2, \ldots, \rvx_n$), and this is denoted as $\MSE(\widehatbtheta, \btheta)$.
\end{definition}

Notice that the MSE quantifies  the average squared difference between the estimator $\widehatbtheta$ and the parameter $\btheta$, providing a reasonable way to assess the performance of an estimator. While other measures, such as the \textit{mean absolute error (MAE)}, defined as $\Exp\big[\normonebig{\widehatbtheta - \btheta}\big]$, are also valid, the MSE has two main advantages: it is mathematically convenient and allows for a useful decomposition into bias and variance.
\index{Bias-variance decomposition}
\begin{lemma}[Bias-variance decomposition]\label{lemma:bias-variance}
For any estimator $\widehatbtheta$ of $\btheta\in\real^p$, the mean squared error of the estimator has the following decomposition:
$$
\begin{aligned}
\MSE(\widehatbtheta, \btheta) &= \normtwo{\Exp[\widehatbtheta] - \btheta}^2 + \Exp\left[\normtwo{\widehatbtheta - \Exp(\widehatbtheta)}^2\right] 
\triangleq \normtwo{\bias(\widehatbtheta, \btheta)}^2 + \sum_{i=1}^{p} \Var[\widehat{\theta}_i].
\end{aligned}
$$
That is, a sum of a bias term and a variance term.
\end{lemma}

\begin{proof}[of Lemma~\ref{lemma:bias-variance}]
Write out the mean squared error
$$
\begin{aligned}
\MSE(\widehatbtheta, \btheta) 
&= \Exp\big[\normtwobig{\widehatbtheta - \btheta}^2\big] 
=\Exp\big[\normtwobig{\widehatbtheta}^2\big] - 2\btheta^\top\Exp\big[\widehatbtheta\big]  + \normtwo{\btheta}^2\\
&=\Exp\big[\normtwobig{\widehatbtheta}^2\big] - \normtwobig{\Exp\big[{\widehatbtheta}\big]}^2+ 
\normtwobig{{\Exp\big[{\widehatbtheta}\big]}- \btheta}^2 ,
\end{aligned}
$$
which completes the proof.
%\begin{equation}
%\begin{aligned}
%\MSE(\widehatbtheta, \btheta) 
%&= \Exp\big[\normtwobig{\widehatbtheta - \btheta}^2\big] 
%%=\Exp\left[\normtwobig{\widehatbtheta -\Exp(\widehatbtheta) +\Exp(\widehatbtheta)- \btheta}^2\right] \\
%=\Exp\left[\big(\widehatbtheta -\Exp(\widehatbtheta) +\Exp(\widehatbtheta)- \btheta\big)^\top\big(\widehatbtheta -\Exp(\widehatbtheta) +\Exp(\widehatbtheta)- \btheta\big)    \right] \\
%&=\normtwobig{\Exp[\widehatbtheta] - \btheta}^2 + \Exp\left[\normtwobig{\widehatbtheta - \Exp(\widehatbtheta)}^2\right] + 2\Exp\left[\big(\widehatbtheta-\Exp(\widehatbtheta)\big)^\top\big(\widehatbtheta-\Exp(\widehatbtheta)\big)\right] \\
%%&=\normtwobig{\Exp[\widehatbtheta] - \btheta}^2 + \Exp\left[\normtwobig{\widehatbtheta - \Exp(\widehatbtheta)}^2\right] + 2\big(\Exp[\widehatbtheta]-\Exp[\widehatbtheta]\big)^\top(\Exp[\widehatbtheta]-\Exp[\widehatbtheta])\\
%&=\normtwobig{\Exp[\widehatbtheta] - \btheta}^2 + \Exp\left[\normtwobig{\widehatbtheta - \Exp(\widehatbtheta)}^2\right],  \nonumber
%\end{aligned}
%\end{equation}
%which completes the proof.
\end{proof}

The bias of an estimator $\widehatbtheta$ of a parameter $\btheta$ is defined the difference between the expected value of the estimator $\widehatbtheta$ and the true parameter value $\btheta$; that is, $\bias(\widehatbtheta, \btheta) = \Exp[\widehatbtheta] - \btheta$. An estimator whose bias is identically equal to $\bzero$ is called an unbiased estimator and satisfies $\Exp[\widehatbtheta] = \btheta$ for all $\btheta$.

\begin{definition}[Biased and unbiased estimators]\label{defintion:biased_unbiased}
Given the estimator $\widehatbtheta$ of a parameter $\btheta$, the quantity $\Exp[\widehatbtheta]-\btheta=\bias(\widehatbtheta,\btheta)$ is called the \textit{bias} of the estimator  $\widehatbtheta$ with respect to the true parameter $\btheta$.
When the bias at some coordinate of $\btheta$ is positive, we have \textit{overestimation}; conversely, when it is negative, we have \textit{underestimation}; when the bias is zero, we refer to it as an \textit{unbiased estimator}.

Suppose both $\widehatbtheta$ and $\widetildebtheta$ are unbiased estimators of an unknown parameter $\btheta$. 
Then, any linear combination of these estimators of the form $\bt = w \widehatbtheta + (1 - w)\widetildebtheta$, where $w$ is any scalar weight, is also an unbiased estimator of $\btheta$.
\end{definition}

Thus, the MSE of an estimator has two components, one measures the \textit{variability of the estimator (precision)} and the other measures  its \textit{bias (accuracy)}. 
An estimator with good MSE performance has both low variance and low bias. To find such an estimator, one must carefully balance and control both sources of error.
While many commonly used estimators are unbiased or approximately unbiased, it's important to note that being unbiased does not necessarily imply a small MSE. In fact, there is often a trade-off between bias and variance. A small increase in bias may lead to a significant reduction in variance, resulting in an overall improvement in MSE.
For an unbiased estimator $\widehatbtheta$, the bias term vanishes, and the MSE simplifies to:
$$
\MSE(\widehatbtheta, \btheta) = \Exp\big[\normtwobig{\widehatbtheta - \btheta}^2\big] = 
\sum_{i=1}^{p} \Var[\widehat{\theta}_i]
\triangleq 
\Var[\widehatbtheta].
$$
In this case, the MSE is simply the sum of the variances of the individual components---that is, the total variance of the estimator.

\begin{example}[Estimators for Gaussian parameters]\label{example:est_gaus_param}
Let $\rx_1, \rx_2, \ldots, \rx_n$ be i.i.d. from $\normal(\mu, \sigma^2)$ with mean $\mu$ and variance $\sigma^2$, then the \textit{sample mean} $\bar{\rx} \triangleq \frac{1}{n}\sum_{i=1}^{n} \rx_i$ is an unbiased estimator for $\mu$, and the \textit{sample variance} $\rS^2\triangleq  \frac{\sum_{i=1}^n (\rx_i - \bar{\rx})^2}{n-1}$ is an unbiased estimator for $\sigma^2$.
To see this, we have 
$$
\Exp[\bar{\rx}]= \frac{\Exp[\rx_1] +\Exp[\rx_2] + \ldots + \Exp[\rx_n]}{n} = \mu.
$$
Therefore, $\bar{\rx}$ is an unbiased estimator. 
Since
$$
\Var[\bar{\rx}] = \Var \left[\frac{\rx_1 + \rx_2 + \ldots + \rx_n}{n}\right] 
= 
\frac{\Var[\rx_1] +\Var[\rx_2] + \ldots + \Var[\rx_n]}{n^2} = \frac{\sigma^2}{n},
$$
the MSE of $\bar{\rx}$ is
$$
\MSE(\bar{\rx}, \mu) = \Exp[(\bar{\rx} - \mu)^2] = \Var[\bar{\rx}] = \frac{\sigma^2}{n}.
$$

It can be shown that $\frac{(n-1)\rS^2}{\sigma^2} \sim \chi^2_{(n-1)}$. Using the properties of the Chi-squared distribution (Definition~\ref{definition:chisquare_dist}), we have
$$
\begin{aligned}
\Exp \left[ \frac{(n-1)\rS^2}{\sigma^2} \right] &= n-1
\quad\implies \quad
\Exp[\rS^2] = \sigma^2;\\
\Var \left[ \frac{(n-1)\rS^2}{\sigma^2} \right] &= 2(n-1) 
\quad\implies \quad
\Var[\rS^2] = \frac{2\sigma^4}{n-1}.
\end{aligned}
$$
Therefore, $ \rS^2 $ is an unbiased estimator for $ \sigma^2 $, and its MSE is:
$$
\MSE(\rS^2, \sigma^2) = \Exp[(\rS^2 - \sigma^2)^2] = \Var[\rS^2] = \frac{2\sigma^4}{n-1}.
$$

An alternative estimator for $ \sigma^2 $ of a Gaussian distribution   is the maximum likelihood estimator:
$$
\widehat{\sigma}^2 \triangleq  \frac{1}{n} \sum_{i=1}^n (\rx_i - \bar{\rx})^2 = \frac{n-1}{n} \rS^2.
$$
Therefore, we have 
$
\Exp[\widehat{\sigma}^2] = \Exp\big[\frac{n-1}{n} \rS^2\big] = \frac{n-1}{n} \sigma^2
$,
such that $ \widehat{\sigma}^2 $ is a biased estimator for $ \sigma^2 $. The variance of $ \widehat{\sigma}^2 $ can also be calculated as
$$
\Var[\widehat{\sigma}^2] = \Var[\frac{n-1}{n} \rS^2] = \frac{(n-1)^2}{n^2} \Var[\rS^2] =  \frac{2(n-1)\sigma^4}{n^2}.
$$
The MSE of $ \widehat{\sigma}^2 $ is then
\begin{align*}
\Exp[(\widehat{\sigma}^2 - \sigma^2)^2] &= \Var[\widehat{\sigma}^2] + (\bias)^2 
= \frac{2(n-1)\sigma^4}{n^2} + \left( \frac{n-1}{n} \sigma^2 - \sigma^2 \right)^2 = \frac{(2n-1)\sigma^4}{n^2}.
\end{align*}
Comparing the two estimators:
$$
\MSE(\widehat{\sigma}^2, \sigma^2) = \frac{(2n-1)\sigma^4}{n^2} < \frac{2n\sigma^4}{n^2} = \frac{2\sigma^4}{n} < \frac{2\sigma^4}{n-1} = \MSE(\rS^2, \sigma^2).
$$
This shows that $ \widehat{\sigma}^2 $ has a smaller MSE than $ \rS^2 $. Therefore, by trading off some bias, we achieve a reduction in overall estimation risk when measured by MSE.
\end{example}

The above example does not necessarily imply that $ \rS^2 $ should be discarded as an estimator of $ \sigma^2 $. 
While $ \widehat{\sigma}^2 $ may have a lower MSE, it systematically underestimates $ {\sigma}^2 $, which might make it less desirable depending on the context or interpretation.

In general, since MSE depends on the true parameter value, there is no single ``best" estimator across all possible values of the parameter. Often, the MSEs of two estimators will cross---one performs better for certain parameter values, while the other performs better for others. Still, such comparisons can offer useful guidance in selecting between estimators.

One way to make the problem of finding a ``best" estimator more manageable is to restrict the class of estimators considered. A common approach is to focus only on unbiased estimators and choose the one with the smallest variance. This leads to the concept of the \textit{best unbiased estimator (BUE)} or the \textit{minimum variance unbiased (MVU) estimator}.
If two estimators, $\widehatbtheta_1$ and $\widehatbtheta_2$, are both unbiased for a parameter $\btheta$, meaning $\Exp[\widehatbtheta_1] = \btheta$ and $\Exp[\widehatbtheta_2] = \btheta$, then their MSEs reduce to their variances. In this case, we prefer the estimator with the smaller variance.

\begin{example}[Estimators for Gaussian parameters, CNT.]
Consider the same setting as in Example~\ref{example:est_gaus_param}. One might ask: what value of $\gamma$ makes the quantity  $\gamma \sum_{i=1}^n (\rx_i - \bar{\rx})^2$ achieve the smallest MSE?
Note that when $\gamma = \frac{1}{n-1}$, we obtain the sample variance $\rS^2$ from Example~\ref{example:est_gaus_param}; whereas when $\gamma = \frac{1}{n}$, we obtain the estimator $\widehat{\sigma}^2$.
Let
$$
\begin{aligned}
\rs^2_\gamma &\triangleq \gamma \sum_{i=1}^n (\rx_i - \bar{\rx})^2 = \gamma (n-1) \rS^2;\\
%\qquad\text{and}\qquad 
\zeta &\triangleq \gamma (n-1).
\end{aligned}
$$
Then
$$
\begin{aligned}
\Exp[\rs^2_\gamma] &= \gamma (n-1) \Exp[\rS^2] = \gamma (n-1) \sigma^2 = \zeta \sigma^2;\\
\Var[\rs^2_\gamma] &= \gamma^2 (n-1)^2 \Var[\rS^2] = \frac{2\zeta^2 }{(n-1)}  \sigma^4.
\end{aligned}
$$
The MSE of $\rs^2_\gamma$ can be obtained as 
\begin{align*}
\MSE(\rs^2_\gamma, \sigma^2) &= \Var[\rs^2_\gamma] + (\bias)^2 = \Var[\rs^2_\gamma] + \left[\Exp[\rs^2_\gamma] - \sigma^2\right]^2 \\
&= \Var[\rs^2_\gamma] + (\zeta \sigma^2 - \sigma^2)^2 = \frac{2\zeta^2 }{(n-1)}  \sigma^4 + (\zeta-1)^2 \sigma^4 \triangleq g(\zeta)\sigma^4, 
\end{align*}
where $g(\zeta) \triangleq \frac{2\zeta^2 }{(n-1)}    + (\zeta-1)^2 $.
The function $ g(\zeta) $ reaches its minimum at $ \zeta = \frac{n-1}{n+1} $.
The minimal MSE is $ \MSE(\rs^2_\gamma, \sigma^2) = \frac{2}{n+1} \sigma^4 $, with $ \gamma(n-1) = \zeta = \frac{n-1}{n+1} $, i.e., $ \gamma = \frac{1}{n+1} $.
Therefore, for $ \gamma = \frac{1}{n+1} $ and $n>1$,  we have:
$$
\MSE(\rs^2_\gamma, \sigma^2) =\frac{2\sigma^4}{n+1} < 
\frac{(2n-1)\sigma^4 }{n^2}  =\MSE(\widehat{\sigma}^2, \sigma^2) 
< \frac{2\sigma^4}{n-1} = \MSE(\rS^2, \sigma^2).
$$
Thus, the estimator with the smallest MSE corresponds to an unbiased estimator.
\end{example}

\begin{example}[Estimators for Laplace parameters]
Let $\rx_1, \rx_2, \ldots, \rx_n$ be i.i.d. random variables with the probability density function $p(x\mid b) \triangleq \frac{1}{2b} \exp \left( -\frac{\abs{x}}{b} \right)$ (Definition~\ref{definition:laplace_distribution}).
The maximum likelihood estimator for $b$, $\widehat{b} = \frac{\sum_{i=1}^n |\rx_i|}{n}$,  is unbiased.
To see this, we  first calculate $\Exp[\abs{\rx}]$ and $\Exp[\abs{\rx}^2]$ as
\begin{align*}
\Exp[\abs{\rx}] &= \int_{-\infty}^{\infty} \abs{x} p(x\mid b)\, dx = \int_{-\infty}^{\infty} \abs{x} \frac{1}{2b} \exp \left( -\frac{\abs{x}}{b} \right) \,dx \\
&= b \int_{0}^{\infty} \frac{x}{b} \exp \left( -\frac{x}{b} \right) d\frac{x}{b} = b \int_{0}^{\infty} t e^{-t} dt = b \Gamma(2) = b
\end{align*}
and
\begin{align*}
\Exp[\abs{\rx}^2] &= \int_{-\infty}^{\infty} \abs{x}^2 p(x\mid b) \,dx = \int_{-\infty}^{\infty} \abs{x}^2 \frac{1}{2b} \exp \left( -\frac{\abs{x}}{b} \right) \,dx \\
&= b^2 \int_{0}^{\infty} \frac{x^2}{b^2} \exp \left( -\frac{x}{b} \right) d\frac{x}{b} = b^2 \int_{0}^{\infty} t^2 e^{-t} dt = b^2 \Gamma(3) = 2b^2.
\end{align*}
Therefore,
$
\Exp[\widehat{b}] = \Exp \left[(\abs{\rx_1} + \ldots + \abs{\rx_n})/n \right] = \frac{\Exp[\abs{\rx_1}] + \ldots + \Exp[\abs{\rx_n}]}{n} = b.
$
Thus, $\widehat{b}$ is an unbiased estimator for $b$.
Since the estimator is unbiased, its MSE is equal to its variance:
\begin{align*}
\MSE(\widehat{b}, b) &= \Exp[(\widehat{b} - b)^2] = \Var[\widehat{b}] = \Var\left[\frac{\abs{\rx_1} + \ldots + \abs{\rx_n}}{n}\right] \\
&= \frac{\Var[\abs{\rx_1}] + \ldots + \Var[\abs{\rx_n}]}{n^2} = \frac{\Var[\abs{\rx_1}])}{n} \\
&= \frac{\Exp[\abs{\rx}^2] - (\Exp[\abs{\rx}])^2}{n} = \frac{2b^2 - b^2}{n} = \frac{b^2}{n}.
\end{align*}
Once again, the large the sample size, the smaller the variance.
\end{example}

\index{Maximum likelihood estimator}
\section{Estimation for OLS}\label{section:mle-gaussian}

%\begin{definition}[Maximum Likelihood Estimator]
%	Let $x_1, x_2, \ldots, x_n$ be a set of i.i.d. random samples from a distribution $F_\theta$ with density $f(x;\theta)$. Then the maximum likelihood estimator (MLE) \footnote{\textbf{Estimation method} vs \textbf{estimator} vs \textbf{estimate}: \textbf{Estimation method} is a general algorithm to produce the estimator. \textbf{An estimate} is the specific value that \textbf{an estimator} takes when observing the specific value, i.e., an estimator is a random variable and the realization of this random variable is called an estimate.} of $\theta$ is $\widehat{\theta}$ such that 
%\begin{equation}
%\mathcalL(\theta) \leq \mathcalL(\widehat{\theta}), \,\, \forall\theta \in \Theta, \nonumber
%\end{equation}
%where $\mathcalL(\theta) = \prod_{i=1}^{n} f(x_i; \theta)$ is the \textbf{likelihood function} for the i.i.d. collection. That is, the MLE of $\theta$ can be obtained by 
%$$
%\widehat{\theta} = \mathop{\arg\max}_{\theta\in \Theta} \mathcalL(\theta). 
%$$
%\end{definition}

\index{Gauss-Markov linear model}
\index{Likelihood function}
\index{Log-likelihood function}
\index{Gaussian disturbance}
As mentioned above, the main focus of this chapter is to derive the MLE for the OLS model (i.e., the Gauss-Markov linear model). 
In this section, we will derive the analytic solutions of the MLE for OLS under both Gaussian and Laplace noise.
\subsection{Gaussian Noise}
From the likelihood arising from Gaussian disturbances, as given in \eqref{equation:likelihood-of-gaussiannoise} or its logarithmic form in \eqref{equation:loglikelihood-of-gaussiannoise}, we can obtain the maximum likelihood estimator for the Gauss-Markov linear model.
\begin{theoremHigh}[MLE for LS under i.i.d. Gaussian disturbance]\label{theorem:mle-gaussian}
Let $\rvy = \bX\bbeta + \bepsilon$, where $\bepsilon \sim \normal(\bzero, \sigma^2 \bI)$, which is known as the \textit{Gauss-Markov linear model} or \textit{Gaussian linear regression model}. 
Assume that $\bX \in \real^{n\times p}$ is fixed and has full rank with $n\geq p$ (i.e., its rank is $p$). Then, the maximum likelihood estimate~\footnote{We will also show this is the best linear unbiased estimate of $\bbeta$ in Theorem~\ref{theorem:gauss-markov}.} of $\bbeta$ is given by 
$$\widehatbbeta = (\bX^\top\bX)^{-1}\bX^\top\by~
\footnote{Note that the  maximum likelihood \textcolor{mydarkblue}{estimate} of $\bbeta$ is denoted by $(\bX^\top\bX)^{-1}\bX^\top\textcolor{mydarkblue}{\by}$ with italic fonts, while its maximum likelihood \textcolor{mydarkblue}{estimator}  is denoted by $(\bX^\top\bX)^{-1}\bX^\top\textcolor{mydarkblue}{\rvy}$ with normal fonts.}
$$ 
for all values of $\sigma^2$. 
And the maximum likelihood estimate~\footnote{We will show this is a biased estimate of $\sigma^2$; the unbiased estimate is discussed in Section~\ref{sec:dist_sse}.} of $\sigma^2$ is given by
$$
\widehat{\sigma}^2=\frac{1}{n}(\by-\bX\widehatbbeta)^\top(\by-\bX\widehatbbeta),
$$
which is equal to the average sum of squares due to error.
\end{theoremHigh}

\begin{proof}[of Theorem~\ref{theorem:mle-gaussian}]
Following Equation~\eqref{equation:likelihood-of-gaussiannoise}, the likelihood of this model under the given parameters $\bbeta$ and $\sigma^2$ is 
\begin{equation}
\mathcalL(\bbeta, \sigma^2) = \frac{1}{(2\pi \sigma^2)^{n/2}} \exp\left\{-\frac{1}{2\sigma^2} (\by-\bX\bbeta)^\top(\by-\bX\bbeta)\right\}, \nonumber
\end{equation}
and the log-likelihood is given by
\begin{equation}
\ell(\bbeta, \sigma^2) = -\frac{1}{2}\left\{ n\ln 2\pi + n\ln \sigma^2 +\frac{1}{\sigma^2}(\by-\bX\bbeta)^\top(\by-\bX\bbeta)  \right\} \nonumber. 
\end{equation}
Maximizing the likelihood is equivalent to maximizing the log-likelihood since the logarithm is a monotonically increasing function.
Thus, the MLE of $\widehatbbeta$ is obtained by solving
\begin{equation}
\widehatbbeta = \mathop{\arg\max}_{\bbeta}{ -(\by-\bX\widehatbbeta)^\top(\by-\bX\widehatbbeta) }=\mathop{\arg\min}_{\bbeta}{ (\by-\bX\widehatbbeta)^\top(\by-\bX\widehatbbeta) }, \nonumber
\end{equation}
where constant terms have been omitted.
To find the solution, we solve the following optimization problem:
$$
\begin{aligned}
\bzero&= \frac{\partial}{\partial \bbeta}(\by-\bX\widehatbbeta)^\top(\by-\bX\widehatbbeta)
\qquad \implies\qquad 
\widehatbbeta= (\bX^\top\bX)^{-1}\bX^\top\by. 
\end{aligned}
$$
Therefore, the MLE of $\bbeta$ is given by $\widehatbbeta=(\bX^\top\bX)^{-1}\bX^\top\by$,  which holds for any value of $\sigma^2$. This result coincides with the OLS estimator.

To find the MLE of the variance parameter $\sigma^2$, consider
$$
\begin{aligned}
\widehat{\sigma}^2 
& = \mathop{\arg\max}_{\sigma^2}\left\{ \mathop{\arg\max}_{\bbeta}{\ell(\bbeta, \sigma^2)}  \right\} 
=\mathop{\arg\max}_{\sigma^2}{ \ell(\widehatbbeta, \sigma^2) } \\
&= \mathop{\arg\max}_{\sigma^2}{-\frac{1}{2}\left\{ n\ln 2\pi + n\ln \sigma^2 +\frac{1}{\sigma^2}(\by-\bX\widehatbbeta)^\top(\by-\bX\widehatbbeta)  \right\}}.
\end{aligned}
$$
By taking the derivative with respect to $\sigma^2$ and setting it to zero, we obtain
$
\widehat{\sigma}^2 = \frac{1}{n}(\by-\bX\widehatbbeta)^\top(\by-\bX\widehatbbeta).
$
It can also be shown that the second partial derivative of the log-likelihood functions are negative, which completes the proof.
\end{proof}

Under Gaussian noise, the MLE of $\widehatbbeta$ is identical  to the OLS estimator because both are derived by minimizing the same sum of squared error $\normtwo{\by-\bX\bbeta}^2=(\by-\bX\bbeta)^\top(\by-\bX\bbeta)$. 
Therefore, in the following sections, we will not distinguish between the MLE and the OLS estimator of  $\bbeta$, and the two terms will be used interchangeably.

In Section~\ref{sec:dist_sse}, we will show that the maximum likelihood estimator   $\frac{1}{n}(\by-\bX\widehatbbeta)^\top(\by-\bX\widehatbbeta)$ of the variance parameter is a \textit{biased estimator} of $\sigma^2$ (Definition~\ref{defintion:biased_unbiased}). An unbiased estimator is given by $\frac{1}{n-p}(\by-\bX\widehatbbeta)^\top(\by-\bX\widehatbbeta)$.

\index{GLS}
\index{Generalized Gauss-Markov linear model}
\index{Gauss-Markov linear model}
\index{Gaussian noise}
\index{Gaussian disturbance}
\paragrapharrow{Gaussian Noise for GLS}
Previously, we consider the linear model $\rvy = \bX\bbeta + \bepsilon$, where $\bepsilon \sim \normal(\bzero, \sigma^2 \bI)$ and $\bX \in \real^{n\times p}$ is fixed and has full column rank with $n \geq p$, i.e., rank is $p$:
\begin{equation}\label{equation:plain_lm}
\rvy \sim \normal(\bX\bbeta, \sigma^2\bI).
\end{equation}
Now, consider a more general case in which the covariance matrix of the noise vector $\bepsilon$ is not diagonal:
\begin{equation}
\bepsilon \sim \normal(\bzero, \sigma^2 \bOmega)
\qquad\implies\qquad 
\rvy \sim \normal(\bX\bbeta, \sigma^2\bOmega).
\end{equation}
where the error covariance matrix $\bOmega$ is symmetric positive definite, and  $ \bOmega\neq \bI $. 
This is known as the \textit{generalized Gauss-Markov linear model} and corresponds to the \textbf{generalized least squares (GLS)} problem defined in \eqref{equation:gls_prob_loss}.

We can factorize  $\bOmega$ into its symmetric square roots $\bOmega = \bOmega^{1/2} \bOmega^{1/2}$ (Theorem~\ref{theorem:unique-factor-pd}). Let $\rvz \triangleq \bOmega^{-1/2}\rvy$, we have 
$$
\Exp[\rvz] =\bOmega^{-1/2} \bX\bbeta
\qquad\text{and}\qquad 
\Cov[\rvz] = \sigma^2\bI. 
$$
In other words, 
\begin{equation}\label{equation:plain_lm_glm}
\rvz\sim \normal(\bOmega^{-1/2}\bX\bbeta, \sigma^2\bI).
\end{equation}
Comparing the form of Equation~\eqref{equation:plain_lm} with the original model in  Equation~\eqref{equation:plain_lm_glm},
we observe that $\rvz$ now follows a standard linear model with i.i.d. Gaussian errors. This allows us to apply ordinary least squares or maximum likelihood estimation on $\rvz$, projecting onto the column space of $\bA\triangleq\bOmega^{-1/2}\bX$, to estimate $\bbeta$.
Using the result from Theorem~\ref{theorem:mle-gaussian}, the maximum likelihood estimators for $\bbeta$ and $\sigma^2$ can be expressed as:
$$
\begin{aligned}
\widehatbbeta&= (\bA^\top\bA)^{-1}\bA^\top\rvz =(\bX^\top  \bOmega^{-1} \bX)^{-1}\bX^\top\bOmega^{-1} \rvy; \\
\widehat{\sigma}^2 &= \frac{1}{n}(\rvz-\bA\widehatbbeta)^\top(\rvz-\bA\widehatbbeta) = \frac{1}{n}(\rvy-\bX\widehatbbeta)^\top \bOmega^{-1} (\rvy-\bX\widehatbbeta).
\end{aligned}
$$
Similarly, other estimators in Table~\ref{table:diff-estitors} can be obtained accordingly.

\index{GLS}
\begin{theoremHigh}[MLE for GLS under full Gaussian disturbance]\label{theorem:gauss_markov_gls}
Consider a Gauss-Markov linear model $ \rvy =\bX\bbeta + \bepsilon $ with $ \bX \in \real^{n\times p} $ of $\rank(\bX) = p$ and symmetric positive definite error covariance matrix $ \Cov[\bepsilon] = \sigma^2 \bOmega \in \real^{n \times n} $, which is known as the \textit{generalized Gauss-Markov linear model}. 
Then, the maximum likelihood estimate of $\bbeta$ is given by 
$$
\widehatbbeta = (\bX^\top  \bOmega^{-1} \bX)^{-1}\bX^\top\bOmega^{-1} \by
$$ 
for all values of $\sigma^2$. 
And the maximum likelihood estimate of $\sigma^2$ is $$
\widehat{\sigma}^2=\frac{1}{n}(\by-\bX\widehatbbeta)^\top \bOmega^{-1} (\by-\bX\widehatbbeta).
$$

%Then the best unbiased linear estimate of $ \bbeta $ is the solution of the solution $ \widehatbbeta $ satisfies the generalized normal equations
%\begin{equation}
%\bX^\top \bOmega^{-1} \bX \widehatbbeta = \bX^\top \bOmega^{-1} \by 
%\end{equation}
%or, equivalently, the orthogonality condition $ \bX^\top \bOmega^{-1} (\by - \bX {\bbeta}) = \bzero $.

\end{theoremHigh}

\index{Biased estimator}
\index{Unbiased estimator}
\index{Newton-Raphson iteration}
\index{Laplace noise}
\subsection{Laplace Noise}
The maximum likelihood estimation  framework can be applied to various types of noise models in the context of linear regression. While the Gaussian noise model is commonly used due to its mathematical convenience and desirable statistical properties, other noise models may be more appropriate depending on the nature of the data. 
One such alternative is the \textit{Laplace noise model}, which assumes zero-mean errors:
$$
\ry_i = \beta_0 +\beta_1x_{i1} + \beta_2 x_{i2} +\ldots +\beta_{p-1} x_{i,p-1} +\epsilon_i, \gapthree \forall\, i\in\{1,2,\ldots,n\},
$$
where $\epsilon_i\sim \laplacedist(0, b)$ for all $i\in\{1,2,\ldots,n\}$ (Definition~\ref{definition:laplace_distribution}).
Compared to the Gaussian distribution, the Laplace distribution has heavier tails, making it more robust to outliers or deviations from normality.
As with the Gaussian case, the likelihood function for the full dataset is the product of individual probability density functions:
\begin{equation}
\mathcalL(\bbeta, b) = \prod_{i=1}^{n}\frac{1}{2b} \exp\left\{-\frac{\abs{y_i-\bx_i^\top\bbeta}}{b}\right\}.
\end{equation}
Once more, we work with the log-likelihood function to simplify calculations:
\begin{equation}
\begin{aligned}
\ell(\bbeta, b) = \ln \mathcalL(\bbeta, b) 
&= \sum_{i=1}^{n}\left(-\ln(2b) -\frac{\abs{y_i-\bx_i^\top\bbeta}}{b}\right)
=-n\ln(2b) - \frac{\normone{\by-\bX\bbeta}}{b}.
\end{aligned}
\end{equation}
Therefore,  the maximum likelihood estimate of $\bbeta$ can be obtained by 
\begin{equation}
\widetilde{\bbeta} = \mathop{\arg\min}_{\bbeta} \normone{\by-\bX\bbeta}.
\end{equation}
To find the maximum likelihood estimators for both $\bbeta$ and $b$, we need to maximize the log-likelihood function over both parameters $(\bbeta, b)$.
However, unlike in the Gaussian case, there is generally no closed-form solution for these estimates. Numerical optimization techniques are therefore required to compute the MLEs.

\paragrapharrow{MLE for $\beta_0$.}
Taking the partial derivative of $\ell(\bbeta, b)$ with respect to $\beta_0$, and setting it to zero gives:
$$
\frac{\partial \ell(\bbeta, b)}{\partial \beta_0}  = \sum_{i=1}^{n}\frac{\sgn(y_i -\bx_i^\top\bbeta)}{b}=0,
$$
where $\sgn(\cdot)$ is the sign function, which takes the value $-1$ if the argument is negative, 0 if the argument is zero, and 1 if the argument is positive.
%The second partial derivative with respective to $\beta_0$ is 
%$$
%\frac{\partial^2 \ell(\bbeta, b)}{\partial \beta_0^2}=
%-\sum_{i=1}^{n}\frac{\sgn(y_i -\bx_i^\top\bbeta)}{b}=0,
%$$

\paragrapharrow{MLE for $\beta_j$ with $j\in\{2,3,\ldots, p\}$.}
Taking the partial derivative of $\ell(\bbeta, b)$ with respect to $\beta_j$ for $j\in\{2,3,\ldots,p\}$, and setting it to zero:
$$
\frac{\partial \ell(\bbeta, b)}{\partial \beta_j}  = \sum_{i=1}^{n}\frac{\sgn(y_i -\bx_i^\top\bbeta) x_{ij}}{b}=0,
$$

\paragrapharrow{MLE for $b$.}
Taking the partial derivative of $\ell(\bbeta, b)$ with respect to $\beta_0$, and setting it to zero:
$$
\frac{\partial \ell(\bbeta, b)}{\partial b}  = \sum_{i=1}^{n}\left(-\frac{1}{b}+\frac{\abs{y_i -\bx_i^\top\bbeta}}{b^2}\right)=0.
$$
That is, 
$$
b = \frac{1}{n}\sum_{i=1}^{n}\abs{y_i -\bx_i^\top\bbeta}.
$$
Therefore, the maximum likelihood estimator for the scale parameter $b$ is simply the average absolute deviation between the observed responses and the predicted responses.

Finding analytical solutions for the MLEs under Laplace-distributed noise is generally more complex than in the Gaussian case. As a result, numerical optimization techniques such as Newton-Raphson iteration, iteratively reweighted least squares (IRLS), or gradient descent are commonly used in practice. See Chapter~\ref{chapter:glm} for more details.

%We will not delve further into the computational implementation details in this book, but many statistical software packages provide built-in tools for performing such optimizations efficiently.

%Finding analytical solutions for the MLEs in the presence of Laplace-distributed noise may be computationally more involved compared to the Gaussian case. Numerical optimization tools  or statistical software packages,  e.g., Newton-Raphson iteration and gradient descent, are often used for practical implementations. And we will not go into the details in this book.

\subsection{Best Linear Unbiased Estimator (BLUE)}\label{section:beta-blue}

Under the assumption of moment conditions on the noise, as opposed to specifying its full distribution (i.e., $\bepsilon\sim\normal(\bzero, \sigma^2\bI)$), we demonstrate that the OLS estimator exhibits a smaller covariance matrix  than any other linear unbiased estimator.  
Using the bias-variance decomposition (Lemma~\ref{lemma:bias-variance}), we find that the bias component is zero for all unbiased estimators. Therefore, the estimator with the smallest covariance matrix becomes the optimal one in terms of mean squared error.

\index{BLUE}
\index{Gauss-Markov}
\index{Best linear unbiased estimator}
\begin{theoremHigh}[Gauss-Markov]\label{theorem:gauss-markov}
Let $\rvy = \bX \bbeta +\bepsilon$ and assume that:
\begin{enumerate}[(i)]
\item $\bX\in \real^{n\times p}$ is fixed and has full column rank with $n\geq p$ (that is, rank is $p$), so that $\bX^\top \bX$ is invertible.

\item $\Exp [\bepsilon \mid  \bX]=\bzero$.

\item $\Cov [\bepsilon \mid  \bX] = \sigma^2 \bI$, where $\bI$ is the $n\times n$ identity matrix, and $\sigma^2$ is a positive constant.
\end{enumerate}
Note that we do not apply the stronger assumption that the noises are i.i.d. from a Gaussian distribution. Then, the OLS estimator $\widehatbbeta = (\bX^\top\bX)^{-1}\bX^\top \rvy$ is the best linear unbiased estimator (BLUE) of $\bbeta$ (stated in the following). That is, for any \textbf{linear unbiased estimator} $\widetilde{\bbeta}$ of $\bbeta$, it holds that
$$
\Cov[\widetilde{\bbeta} \mid \bX] - \Cov[\widehatbbeta\mid \bX] \succeq \bzero.
$$
\end{theoremHigh}

\begin{proof}[of Theorem \ref{theorem:gauss-markov}]
Let $\widetilde{\bbeta}$ be any linear estimator of $\bbeta$, which could be expressed as $\widetilde{\bbeta} = \bQ \rvy$. 
Specifically, in the case of OLS, we have $\bQ = (\bX^\top\bX)^{-1}\bX^\top$. Using the unbiased assumption, we get 
$$
\begin{aligned}
\bbeta &= \Exp[\widetilde{\bbeta} \mid \bX] = \Exp[\bQ \rvy\mid \bX] = \Exp[\bQ \bX \bbeta + \bQ \bepsilon \mid \bX]
=\bQ \bX\bbeta,
\end{aligned} 
$$
which yields $(\bQ\bX -\bI)\bbeta = \bzero $. We conclude that the null space of $(\bQ\bX -\bI)$ is the entire space $\real^p$, and thus, $\bQ \bX = \bI$. The variance is then
$$
\begin{aligned}
\Cov[\bQ\rvy \mid \bX] &= \bQ \Cov[\rvy \mid \bX] \bQ^\top = \bQ \Cov[\bepsilon \mid \bX] \bQ^\top
= \bQ (\sigma^2 \bI) \bQ,
\end{aligned}
$$
where the first equality follows from the fact of covariance under linear transformation by $\Cov[\bQ \rvy] = \bQ \Cov[\rvy]\bQ^\top$.
Then, it follows that 
$$
\begin{aligned}
&\gap \Cov[\widetilde{\bbeta} \mid \bX] - \Cov[\widehatbbeta \mid \bX]
=\bQ (\sigma^2 \bI) \bQ - \sigma^2(\bX^\top\bX)^{-1} \\
&\stackrel{\dag}{=} \sigma^2\left[\bQ \bQ^\top - \bQ\bX(\bX^\top\bX)^{-1}\bX^\top\bQ^\top\right]  
= \sigma^2\bQ(\bI - \bH)\bQ^\top \\
&= \sigma^2\bQ(\bI - \bH)(\bI - \bH)^\top \bQ^\top 
\succeq \bzero, 
\end{aligned}
$$
where the equality ($\dag$) follows from the fact that $\bQ\bX=\bI$, and  the last equality follows from the fact that both $\bH$ and $\bI -\bH$ are idempotent. 
Note that $\Cov[\widehatbbeta \mid \bX]=\sigma^2(\bX^\top\bX)^{-1}$ will be shown in Theorem~\ref{theorem:samplding_dist_lse_gaussian_mom}. This completes the proof.
\end{proof}

Since there is a huge variety of candidate distributions for $\bepsilon$ that would be compatible with the property $\Cov[\bepsilon \mid \bX]=\sigma^2\bI$, we cannot say very much about the exact distribution of $\widehatbbeta$ or $\widehatbbeta-\bbeta$. However, under mild regularity  conditions, the distribution of $\widehatbbeta$ is asymptotically normal for large sample sizes; see Theorem~\ref{theorem:larg-sample-hat-beta}:
\begin{equation}\label{equation:large_asy_tmp}
\text{For large $n$, it follows that $\widehatbbeta \sim \normal(\bbeta, \sigma^2 (\bX^\top\bX)^{-1})$}.
\end{equation}
This asymptotic distribution is identical to the one obtained when $\bepsilon$ is normally distributed (as shown in Theorem~\ref{theorem:samplding_dist_lse_gaussian}), demonstrating that OLS enjoys favorable statistical properties even without assuming normality.

The  theorem  shows that the OLS estimator achieves the smallest variance compared to other linear unbiased estimators. 
Subsequently, in the bias-variance decomposition (Lemma~\ref{lemma:bias-variance}),  we show that the mean squared error between the estimator of $\widehatbbeta$ and the true parameter $\bbeta$ is a sum of a bias term and a variance term. 
Since OLS is unbiased, its MSE equals its variance. Among all  linear unbiased estimators, OLS has the smallest variance, making it optimal in terms of MSE as well; that is, OLS estimator is the \textit{best linear unbiased estimator (BLUE)}.

While we have established that the OLS estimator is the best linear unbiased estimator  available ``on hand", one might question whether there exists an imaginary or theoretical estimator with less variance  (not necessarily unbiased). This will be addressed in the next section.

\index{Minimum variance unbiased estimator}
\subsection{Minimum Variance Unbiased (MVU) Estimator}\label{section:mvu_esti}

In parameter estimation problems, we obtain information about an unknown parameter from a sample of data drawn from an underlying probability distribution. A natural question arises: How much information can a given sample provide about the unknown parameter? This section introduces a measure of such information. We will also see that this information measure can be used to establish bounds on the variance of estimators and to approximate the sampling distribution of an estimator based on a large sample; see Theorem~\ref{theorem:asymp_of_mle}.

\index{Score function}
\index{Fisher information}
\subsection*{Fisher Information}
For brevity, we consider a random variable $\rx$ for which the p.d.f. or p.m.f. is $p(x\mid \theta)$, where $\theta$ is an unknown parameter belonging to a parameter space $\Theta$.
Intuitively, if an event has a  small probability, its occurrence provides significant information. 
For a random variable $\rx \sim p(x\mid \theta)$, if $\theta$ were the true value of the parameter, the likelihood function should be large, or equivalently, the derivative log-likelihood function should be close to zero. This is the fundamental idea behind maximum likelihood estimation; see Section~\ref{section:mle_method}.
We define $\ell(\theta; x) \triangleq \ln p(x\mid \theta)$ as the log-likelihood function w.r.t. $\theta$, and it follows that 
\begin{equation}
\ell'(\theta; x) = \frac{\partial}{\partial \theta} \ln p(x\mid \theta) = \frac{p'(x\mid \theta)}{p(x\mid \theta)},
\end{equation}
where $p'(x\mid \theta)$ denotes the derivative of $p(x\mid \theta)$ with respect to $\theta$. 
Note that $\ell'(\theta; \rx)$, written with normal font $\rx$, denotes a random variable; whereas $\ell'(\theta; x)$, written with an italic font $x$, denotes a particular realization of that variable.
Similarly, we denote the second order derivative of $p(x\mid \theta)$ with respect to $\theta$ as $p''(x\mid \theta)$.

\paragrapharrow{Fisher information for a single random variable.}
From the above discussion, if $\ell'(\theta; \rx)$ is close to zero, the observed value is expected, and thus does not convey much information about $\theta$.
On the other hand, when the absolute values $\abs{\ell'(\theta; \rx)}$ or $\abs{\ell''(\theta; \rx)}$ are large, the random variable provides substantial information about $\theta$. 
Therefore, we can use $[\ell'(\theta; \rx)]^2$ to measure the amount of information provided by $\rx$. However, since $\rx$ itself is a random variable, we should consider the average case. Thus, we introduce the following definition:

\begin{definition}[Fisher information (for $\theta$)]\label{definition:fish_inf_theta}
The \textit{Fisher information} contained in the random variable $\rx$ discussed above is defined as:
\begin{equation}
\text{(F1)}:\qquad 
\colorbox{gray!10}{
	$
\fisher(\theta) 
= \Exp_{\theta} \left[\big(\ell'(\theta; \rx)\big)^2\right] $
}
= \int [\ell'(\theta; x)]^2 p(x\mid \theta) \,dx.
\end{equation}
\end{definition}
Assuming we can interchange differentiation and integration, we obtain:
$$
\begin{aligned}
\int p'(x\mid \theta)\,dx &= \frac{\partial}{\partial \theta} \int p(x\mid \theta)\,dx = 0;\\
\int p''(x\mid \theta)\,dx &= \frac{\partial^2}{\partial \theta^2} \int p(x\mid \theta)\,dx = 0.
\end{aligned}
$$
It then follows that:
$$
\Exp_{\theta}[\ell'(\theta; \rx)] = \int \ell'(\theta; x)p(x\mid \theta)\,dx = \int \frac{p'(x\mid \theta)}{p(x\mid \theta)} p(x\mid \theta)\,dx = \int p'(x\mid \theta)\,dx = 0.
$$
Hence, the definition of Fisher information in Definition~\ref{definition:fish_inf_theta} can be equivalently expressed as
\begin{equation}
\text{(F2)}:\qquad 	
\colorbox{gray!10}{
$\fisher(\theta) = \Var_{\theta} \left[ \ell'(\theta; \rx) \right]$
}
.
\end{equation}
Moreover, observe that:
\begin{equation}
\ell''(\theta; x) = \frac{\partial}{\partial \theta} \left[ \frac{p'(x\mid \theta)}{p(x\mid \theta)} \right] = \frac{p''(x\mid \theta)p(x\mid \theta) - [p'(x\mid \theta)]^2}{[p(x\mid \theta)]^2} = \frac{p''(x\mid \theta)}{p(x\mid \theta)} - [\ell'(\theta; x)]^2.
\end{equation}
Taking expectations:
$$
\begin{aligned}
\Exp_{\theta}[\ell''(\theta; \rx)] 
&= \int \left[ \frac{p''(x\mid \theta)}{p(x\mid \theta)} - [\ell'(\theta; x)]^2 \right] p(x\mid \theta)\,dx \\
&= \int p''(x\mid \theta)\,dx - \Exp_{\theta} \big[\big(\ell'(\theta; \rx)\big)^2\big]
= -\fisher(\theta).
\end{aligned}
$$
This leads to another equivalent expression for the Fisher information:
\begin{equation}
\text{(F3)}:\qquad 		
\colorbox{gray!10}{
$\fisher(\theta) = -\Exp_{\theta}[\ell''(\theta; \rx)]$
}
 = -\int \left[ \frac{\partial^2}{\partial \theta^2} \ln p(x\mid \theta) \right] p(x\mid \theta)\,dx.
\end{equation}

\paragrapharrow{Fisher information for multiple samples.}
Now suppose that we are given a random sample $\rx_1, \rx_2, \ldots, \rx_n$ drawn from a distribution with the p.d.f. or p.m.f.  $p(x\mid \theta)$, where the value of the parameter $\theta$ is unknown. 
We now aim to calculate how much information this random sample provides about the parameter $\theta$.
We  denote the joint p.d.f. of $\rx_1, \rx_2, \ldots, \rx_n$ as
$$
p_n(\bx \mid \theta) = \prod_{i=1}^{n} p(x_i\mid \theta).
$$
The corresponding joint log-likelihood function and its derivative (with respect to $\theta$) are given by:
\begin{align}
\ell_n(\theta; \bx) &= \ln p_n(\bx\mid \theta) = \sum_{i=1}^{n} \ln p(x_i\mid\theta) = \sum_{i=1}^{n} \ell(\theta; x_i);\\
\ell'_n(\theta; \bx) &= \frac{p'_n(\bx\mid \theta)}{p_n(\bx\mid \theta)} =  \sum_{i=1}^{n} \ell'(\theta; x_i). \label{equation:fis_mul_grad}
\end{align}

Analogous to the Fisher information defined for a single observation $\rx$ in Definition~\ref{definition:fish_inf_theta}, 
we define the Fisher information $\fisher_n(\theta)$ contained in the random sample $\rx_1, \rx_2, \ldots, \rx_n$ as
$$
(\text{F1}):\qquad 
\colorbox{gray!10}{$\fisher_n(\theta) = \Exp_{\theta} \left[ \big(\ell'_n(\theta; \rvx)\big)^2\right] $}
= \int \ldots \int [\ell'_n(\theta; \bx)]^2 p_n(\bx\mid \theta) dx_1 \ldots dx_n.
$$
which is an $n$-dimensional integral. Assuming that differentiation and integration can be interchanged, we obtain:
$$
\begin{aligned}
\int p'_n(\bx\mid \theta) d\bx  &= \frac{\partial}{\partial \theta} \int p_n(\bx\mid \theta) d\bx = 0;\\
\int p''_n(\bx\mid \theta) d\bx &= \frac{\partial^2}{\partial \theta^2} \int p_n(\bx\mid \theta) d\bx = 0.
\end{aligned}
$$
It then follows that:
\begin{equation}\label{equation:fis_mul_grazero}
\Exp_{\theta}[\ell'_n(\theta; \rvx)] = \int \ell'_n(\theta; \bx) p_n(\bx\mid \theta) d\bx = \int \frac{p'_n(\bx\mid \theta)}{p_n(\bx\mid \theta)} p_n(\bx\mid \theta) d\bx = \int p'_n(\bx\mid \theta) d\bx = 0.
\end{equation}
Therefore, the Fisher information for the sample $\rx_1, \rx_2, \ldots, \rx_n$ can be expressed as:
$$
\begin{aligned}
(\text{F2 \& F3}):\qquad 
\colorbox{gray!10}{$\fisher_n(\theta) = \Var_{\theta} \left[ \ell'_n(\theta; \rvx) \right]
= -\Exp_{\theta} \left[ \ell''_n(\theta; \rvx) \right]$
}.
\end{aligned}
$$
From the definition of $\ell_n(\theta; \bx)$, it follows that
$\ell''_n(\theta; \bx) = \sum_{i=1}^{n} \ell''(\theta; x_i)$,
whence we have
$$
(\text{F4}):\qquad 
\fisher_n(\theta) = -\Exp_{\theta} \left[ \ell''_n(\theta; \rvx) \right] 
= -\Exp_{\theta} \left[ \sum_{i=1}^{n} \ell''(\theta; \rx_i) \right] 
= -\sum_{i=1}^{n} \Exp_{\theta} \left[ \ell''(\theta; \rx_i) \right] = n \fisher(\theta).
$$
In other words, the Fisher information contained in a random sample of size $n$ is simply $n$ times the Fisher information obtained from a single observation.

\paragrapharrow{Fisher information for multiple parameters.}
Now suppose the distribution model involves more than one parameter. That is, consider a random variable $\rx \sim p(x\mid \btheta)$, where $\btheta = [\theta_1, \theta_1, \ldots, \theta_p]^\top$  is a vector of unknown parameters. 
We denote the log-likelihood function as
$
\ell(\btheta) = \ln p(x\mid \btheta).
$
The first-order derivative (gradient) of $\ell(\btheta)$ with respect to $\btheta$ is a $p$-dimensional vector, known as the \textit{score function}, given by
$$
\frac{\partial \ell(\btheta)}{\partial \btheta} 
= \left[\frac{\partial \ell(\btheta)}{\partial \theta_1}, \ldots, \frac{\partial \ell(\btheta)}{\partial \theta_p}\right]^\top,
$$
The second-order derivative (Hessian) of $\ell(\btheta)$ with respect to $\btheta$ is a $p \times p$ matrix defined as 
$$
\frac{\partial^2 \ell(\btheta)}{\partial \btheta^2} = \left[ \frac{\partial^2 \ell(\btheta)}{\partial \theta_i \partial \theta_j} \right], \quad \forall\, {i=1,2,\ldots,p; j=1,2,\ldots,p}.
$$
We define the Fisher information matrix as
\begin{equation}\label{equation:fish_mul_defi}
\colorbox{gray!10}{
$\fisher(\btheta) = \Exp \left[ \frac{\partial \ell(\btheta)}{\partial \btheta} \left( \frac{\partial \ell(\btheta)}{\partial \btheta} \right)^\top \right] = \Cov \left[ \frac{\partial \ell(\btheta)}{\partial \btheta} \right] = -\Exp \left[ \frac{\partial^2 \ell(\btheta)}{\partial \btheta^2} \right]$
}.
\end{equation}
Since the covariance matrix is symmetric and positive semi-definite, these properties hold for the Fisher information matrix as well.
The Fisher information for $n$ samples is similar given by $\fisher_n(\btheta) = n\fisher(\btheta)$.

\begin{example}[Fisher information for normal distribution]
Consider  a Gaussian distribution $\normal(\mu, \sigma^2)$, we have
$$
\btheta = [\mu, \sigma^2]^\top
\qquad \text{and}\qquad 
\ell(\btheta) = -\frac{1}{2} \ln(2\pi \sigma^2) - \frac{(x-\mu)^2}{2\sigma^2}.
$$
Therefore, the gradient and Hessian are
$$
\begin{aligned}
\frac{\partial \ell(\btheta)}{\partial \btheta} 
&= \left[\frac{\partial \ell(\btheta)}{\partial \mu}, \frac{\partial \ell(\btheta)}{\partial \sigma^2}\right]^\top 
= \left[\frac{x-\mu}{\sigma^2}, -\frac{1}{2\sigma^2} + \frac{(x-\mu)^2}{2(\sigma^2)^2}\right]^\top;\\
\frac{\partial^2 \ell(\btheta)}{\partial \btheta^2} 
&= \begin{bmatrix}
-\frac{1}{\sigma^2} & -\frac{x-\mu}{(\sigma^2)^2} \\
-\frac{x-\mu}{(\sigma^2)^2} & \frac{1}{2(\sigma^2)^2} - \frac{(x-\mu)^2}{(\sigma^2)^3}
\end{bmatrix}.
\end{aligned}
$$
For $\rx \sim \normal(\mu, \sigma^2)$, since $\Exp[\rx - \mu] = 0$ and $\Exp[(\rx - \mu)^2] = \sigma^2$, we can easily get the Fisher information matrix as
$
\fisher(\btheta) 
= -\Exp \left[ \frac{\partial^2 \ell(\btheta)}{\partial \btheta^2} \right] 
= \begin{bmatrixfoot}
	\frac{1}{\sigma^2} & 0 \\
	0 & \frac{1}{2\sigma^4}
\end{bmatrixfoot}.
$
\end{example}

\subsection*{Cram\'er-Rao Lower Bound (CRLB)}
Suppose that we have a random sample $\rx_1, \rx_2, \ldots, \rx_n$ drawn from a distribution with the p.d.f. or p.m.f. given by $p(x \mid \theta)$, where the value of the parameter $\theta$ is unknown. We will demonstrate  how to use Fisher information to determine the lower bound for the variance of an estimator of $\theta$.

Let $\widehat{\theta} = T(\rx_1, \rx_2, \ldots, \rx_n) = T(\rvx)$ be an arbitrary estimator of $\theta$. Assume that  $\Exp_{\theta}[\widehat{\theta}] = m(\theta)$, and that the variance of $\widehat{\theta}$ is finite. 
Consider the random variable $\ell'_n(\rvx\mid\theta)$ defined in \eqref{equation:fis_mul_grad}, it was shown in \eqref{equation:fis_mul_grazero} that $\Exp_{\theta}[\ell'_n(\rvx\mid\theta)] = 0$. Therefore, the covariance between $\widehat{\theta}$ and $\ell'_n(\theta; \rvx)$ is
$$
\small
\begin{aligned}
\Cov_{\theta}[\widehat{\theta}, \ell'_n(\theta; \rvx)] 
&= \Exp_{\theta} 
\left[
\big(\widehat{\theta} - \Exp_{\theta}[\widehat{\theta}]\big) 
\Big(\ell'_n(\rvx\mid\theta) - \Exp_{\theta}[\ell'_n(\rvx\mid\theta)]\Big)
\right] 
= \Exp_{\theta} \left[ \big(\widehat{\theta} - m(\theta)\big) \ell'_n(\theta; \rvx) \right]\\
&= \Exp_{\theta} \left[\big(T(\rvx) - m(\theta)\big)\ell'_n(\theta; \rvx)  \right]  
=\Exp_{\theta} \left[T(\rvx)\ell'_n(\theta; \rvx)  \right]\\
&= \int \ldots \int T(\bx) \ell'_n(\theta; \bx) p_n(\bx\mid \theta) dx_1 \ldots dx_n 
\stackrel{\dag}{=} \int \ldots \int T(\bx) p'_n(\bx\mid \theta) dx_1 \ldots dx_n  \\
&= \frac{\partial}{\partial \theta} \int \ldots \int T(\bx) p_n(\bx\mid \theta) dx_1 \ldots dx_n 
= \frac{\partial}{\partial \theta} \Exp_{\theta}[\widehat{\theta}] = m'(\theta),
\end{aligned}
$$
where the equality ($\dag$) follows from \eqref{equation:fis_mul_grad}.
By the Cauchy-Schwartz inequality (Lemma~\ref{lemma:cs_ineq_exp}) and the definition of $\fisher_n(\theta)$, we obtain:
$$
\begin{aligned}
[m'(\theta)]^2
&= \left\{ \Cov_{\theta}[\widehat{\theta}, \ell'_n(\theta; \rvx)] \right\}^2 
= \Exp_{\theta} \left[ \big(\widehat{\theta} - m(\theta)\big) \ell'_n(\theta; \rvx) \right]\\
&\leq \Var_{\theta}[\widehat{\theta}] \Var_{\theta}[\ell'_n(\theta; \rvx)] = \Var_{\theta}[\widehat{\theta}] \fisher_n(\theta)
= n \fisher(\theta) \Var_{\theta}[\widehat{\theta}].
\end{aligned}
$$
Thus, we arrive at the following lower bound on the variance of any estimator $\widehat{\theta}$:
\begin{equation}
	\Var_{\theta}[\widehat{\theta}] \geq \frac{[m'(\theta)]^2}{n \fisher(\theta)}.
\end{equation}
If we denote $b(\theta) = \Exp[T] - \theta$ as the bias of the estimator $T$, the inequality can be equivalently denoted as 
\begin{equation}
\Var_{\theta}[\widehat{\theta}] \geq \frac{[m'(\theta)]^2}{n \fisher(\theta)} \equiv \frac{[b'(\theta)+1]^2}{n \fisher(\theta)}.
\end{equation}
This is is called the the \textit{Cram\'er-Rao lower bound} or the \textit{information inequality}, in honor of the Swedish statistician H. Cram\'er and Indian statistician C. R. Rao who independently developed this inequality during the 1940s. The information inequality shows that as $\fisher(\theta)$ increases, the variance of the estimator decreases,  implying higher quality of estimation; that is why the quantity $\fisher(\theta)$ is referred to as ``information."

If $\widehat{\theta}$ is an unbiased estimator, then $b(\theta) =0$, $b'(\theta) = 0$. Hence, by the information inequality, for an unbiased estimator $\widehat{\theta}$, $\Var_{\theta}[\widehat{\theta}] \geq \frac{1}{n \fisher(\theta)}$.
Under certain regularity  conditions, no other unbiased estimator of the parameter $\theta$ based on an i.i.d. sample of size $n$ can have a variance smaller than CRLB.
Formally, we formulate the result in the following theorem.
\index{Cram\'er-Rao lower bound}
\begin{theoremHigh}[Cram\'er-Rao lower bound (CRLB)]\label{theorem:crlb}
Let $x_1, x_2, \ldots, x_N$ be an i.i.d. sample from a regular parametric model $p(\cdot \mid\theta)$ (i.e., the model cannot switch between continuous and discrete depending on the value of $\theta$), $\Theta \in \real $. Let $T: \mathcal{X}^N \rightarrow \Theta$ be an estimator of $\theta$, for all $N$. Assume that:
\begin{enumerate}
\item $\Var[T] < \infty$, for all $\theta \in \Theta$;

\item $\frac{\partial}{\partial \theta}\left[\int_{\mathcal{X}^N} p_{\rx_1, \ldots, \rx_N}(x_1, \ldots, x_N;\theta)dx \right]= \int_{\mathcal{X}^N} \frac{\partial}{\partial \theta} p_{\rx_1, \ldots, \rx_N}(x_1, \ldots, x_N; \theta)dx$;

\item  $\frac{\partial}{\partial \theta}\left[\int_{\mathcal{X}^N} T(x_1,\ldots,x_N) p_{\rx_1, \ldots, \rx_N}(x_1, \ldots, x_N;\theta)dx \right]$ = 

$ \qquad \qquad \qquad\; \int_{\mathcal{X}^N} T(x_1,\ldots,x_N) \frac{\partial}{\partial \theta} p_{\rx_1, \ldots, \rx_N}(x_1, \ldots, x_N; \theta)dx$. 
\end{enumerate}
If we denote the bias of $T$ by $b(\theta) = \Exp[T(\rx_1, \rx_2, \ldots, \rx_N)] - \theta$, then it holds that $b(\theta)$ is differentiable, and 
\begin{equation}
\Var[T(\rx_1, \rx_2, \ldots, \rx_N)] \geq  \frac{\left(b' (\theta )+1\right)^2}{N \cdot \Exp \left[\frac{\partial}{\partial\theta} \ln p(\rx_1\mid  \theta)\right]^2}. \nonumber
\end{equation}
That is, the variance has a lower bound.
\end{theoremHigh}

According to the Gauss-Markov theorem, the OLS estimator stands as the best linear unbiased estimator attainable. 
Yet, the gap between this estimator and the theoretical limit of a linear estimator remains unclear.  
Specifically, we seek insights into the minimum variance achievable for {linear estimators} and the proximity of the OLS variance to this theoretical limit.
This question is answered by CRLB.
Note that for this discussion, we assume the additional condition that the noise follows a Gaussian distribution.

By repeating the realization of $\rvy$ from $\bX$ for $N$ times with the same parameters, we find the variance of the OLS estimator attains the bound of CRLB for $\bbeta$. We state the conclusion in the following theorem.

\index{MVU}
\index{Minimum variance unbiased estimator}
\begin{theoremHigh}[Minimum variance unbiased (MVU) estimator]\label{theorem:least-blue-1d}
Let $\rvy = \bX\bbeta + \bepsilon$, where $\bepsilon \sim \normal(\bzero, \sigma^2 \bI)$. 
And assume that $\bX\in \real^{n\times p}$ is fixed and has full rank with $n\geq p$ (i.e., rank is $p$ so that $\bX^\top \bX$ is invertible).
Then, the OLS estimator $\widehatbbeta = (\bX^\top\bX)^{-1}\bX^\top \rvy$ attains the bound of CRLB for $\bbeta$.

\end{theoremHigh}

\begin{proof}[of Theorem~\ref{theorem:least-blue-1d}]
For simplicity, we only prove the one-dimensional case. Interesting readers can replicate  the process to find the proof for the high-dimensional CRLB. 
Referring again to  Equation~\eqref{equation:likelihood-of-gaussiannoise}, the likelihood of this model under $\bbeta$ and $\sigma^2$ is 
\begin{equation}
\begin{aligned}
p(\by\mid \bX,\bbeta)	 &=\normal(\bX\bbeta, \sigma^2\bI)  
= \frac{1}{(2\pi\sigma^2)^{n/2}} \exp\left\{-\frac{1}{2\sigma^2} (\by-\bX\bbeta)^\top(\by-\bX\bbeta)\right\}, \\
\ln p(\by\mid \bX,\bbeta) &= -\frac{n}{2} \ln(2\pi \sigma^2) - \frac{1}{2\sigma^2}(\by^\top\by-2\by^\top\bX\bbeta+\bbeta^\top\bX^\top\bX\bbeta),\\
\frac{\partial \ln p(\by\mid \bX,\bbeta)}{\partial \bbeta} &=\frac{1}{\sigma^2}(\bX^\top\by - \bX^\top\bX\bbeta).
\nonumber
\end{aligned}	
\end{equation}
For one-dimensional inputs $\bX=\bx$, we have 
\begin{equation}
\begin{aligned}
\frac{\partial \ln p(\rvy\mid \bx,\beta)}{\partial \beta} &=\frac{1}{\sigma^2}(\bx^\top\rvy - \bx^\top\bx\beta), \\
\left[\frac{\partial \ln p(\rvy\mid \bx,\beta)}{\partial \beta}\right]^2 &= \frac{1}{\sigma^2} \big(\bx^\top\rvy\rvy^\top\bx - 2\beta\rvy^\top\bx\bx^\top\bx + \beta^2 \bx^\top\bx\bx^\top\bx\big), \\
\Exp\left[\frac{\partial \ln p(\rvy\mid \bx,\beta)}{\partial \beta}\right]^2 &=\frac{\bx^\top\bx}{\sigma^2},
\nonumber
\end{aligned}	
\end{equation}
where the last equation follows from the fact that 
$\Cov[\rvy, \rvy] = \sigma^2\bI = \Exp[\rvy\rvy^\top] - \beta^2\bx\bx^\top. $

%\begin{figure}[h!]
%\centering
%\includegraphics[width=0.9\textwidth]{figures/crlb.pdf}
%\caption{Color $n$ balls in a bottle and repeat the procedure $N$ times. For each ball, we use different algorithms to decide the color of the ball, e.g., the weight of the ball.}
%\label{fig:crlb-ball}
%\end{figure}

Suppose now we have $N$ realizations of $\by_i$ from the same and fixed input matrix $\bX$:
$$
\rvy_i = g(\bX)+\epsilon_i, \gap \forall i\in \{1,2,\ldots, N\}.
$$
Note the difference between $N$ and $n$. For $n$, it means we have $n$ fixed samples of $\bx_j$ on hand (i.e., $\bX\in \real^{n\times p}$). For $N$, we use the same fixed $n$ samples to realize $N$ different outputs $\by_i$ (since we assume Gaussian noise disturbance that gives rises to the likelihood). 
In other words, for one-dimensional case, we have $N$ such $\by_i$'s, and each $\by_i \in \real^n$, $\forall$ $i\in$ $\{1, 2, \ldots, N\}$. 
%The situation can be described by coloring $n$ balls in a bottle. For each ball, we use different algorithms to decide the color of the ball, e.g., the weight of the ball. Now for the same $n$ balls, we repeat the coloring procedure for $N$ times, each result will be in a bottle.

So the variance has the relationship $\Var[\widehatbbeta] \geq \frac{\sigma^2}{N (\bx^\top\bx)}$, as the bias term $b' (\beta )=0$ for an unbiased estimator.
And since we have proved in Theorem~\ref{theorem:samplding_dist_lse_gaussian}, $\Var[\widehat{\beta}_i] =\frac{\sigma^2}{ (\bx^\top\bx)} $. Then, repeating $N$ times results in $\Var[\widehatbbeta] =\frac{\sigma^2}{N (\bx^\top\bx)} $. This completes the proof.
\end{proof}

It follows that, particularly for large sample sizes $N$, the OLS estimator for $\bbeta$ achieves performance equivalent to the theoretical optimum. This once again explains why the OLS method is fundamental to linear model estimation.

\index{Distribution theory}
\section{Distribution Theory for OLS}
%We introduce distribution theories relevant to least squares approximation under Gaussian disturbance.
%Especially, we want to study the sampling distribution for the estimators related to the LS model, which can help understand their precision, build confidence intervals and test hypothesis.
%Questions may arise regarding the bias of this model and, given a bias, whether we can minimize the variance of this estimator. 
%Notably, we can demonstrate that the bias of the least squares estimator is zero, and its variance is upper-bounded.

%Questions can be posed that what the bias of this model is and for a given bias, can we make the variance of this estimator arbitrarily small? Actually, we can prove the bias of the least squares estimator is zero 
%and the variance of it is upper bounded.
%and there is a bound on the variance of it.

We introduce distributional theory relevant to least squares estimation under Gaussian noise. In particular, we aim to study the sampling distribution of estimators related to the least squares model. This allows us to understand their precision, construct confidence intervals, and perform hypothesis testing.

When applying the least squares estimation or maximum likelihood estimator of $\widehatbbeta = (\bX^\top\bX)^{-1}\bX^\top \rvy$ (Theorem~\ref{theorem:mle-gaussian}), we can directly compute the expectation of the estimator as $\Exp[\widehatbbeta] = \Exp[ (\bX^\top\bX)^{-1}\bX^\top \rvy]$ (details provided below).
Under the least squares probability model, we also assume that $\rvy = \bX \bbeta +\bepsilon$, where $\bepsilon$ follows some probability distribution. 
Substituting into the estimator of $\bbeta$, therefore, we have
\begin{equation}
\widehatbbeta = (\bX^\top\bX)^{-1}\bX^\top (\bX \bbeta +\bepsilon). \nonumber
\end{equation}
The distribution theory of least squares revolves around this expression.

\subsection{Mathematical Notations}
In statistical modeling, since the sample is all the information we have, any analysis we perform must be based on it---that is, as a function of the observed data, say $T(\rx_1, \rx_2, \ldots, \rx_n)$. Such a function is called a \textit{statistic} or an estimator we have termed previously.

\begin{definition}[Statistic]
Let $\mathcalX$ be a sample space. For $n \geq 1$, a statistic is a function $T: \mathcalX^n \to \real$.
\end{definition}

Notice that the function $T$ must not depend on the unknown parameter $\theta$, since we do not know the latter. If the function $T$ depends  on $\theta$, then it cannot be considered a statistic.
In general, a statistic contains less information about $\theta$ than the full dataset $(\rx_1, \rx_2, \ldots, \rx_n)$. However, for certain models, we can find a statistic $T$ such that $T(\rx_1, \rx_2, \ldots, \rx_n)$ retains all the relevant information about $\theta$ contained in the original data. Such a statistic is called a \textit{sufficient statistic}, because it suffices to work with $T(\cdot)$ instead of the full dataset.

\begin{definition}[Sufficiency]\label{definition:sufficiency}
Let $\rx_1, \rx_2, \ldots, \rx_n \stackrel{i.i.d.}{\sim} f_{\theta}$. A statistic $T: \mathcalX^n \to \real$ is called \textit{sufficient} for the parameter $\theta$, if the conditional probability $\Pr[\rx_1 \leq x_1, \ldots, \rx_n \leq x_n \mid  T = t]$ does not depend on $\theta$, for all $[x_1, \ldots, x_n]^\top \in \real^n$ and all $t \in \real$.
\end{definition}

Intuitively, once we know the value of $T(\rx_1, \rx_2, \ldots, \rx_n)$, the conditional distribution of the data no longer provides any additional information about $\theta$. Therefore, knowing the full dataset beyond the value of $T$ does not help us determine which value of 
$\theta$ generated the data.
This definition is often difficult to verify directly. However, the following equivalent condition is usually easier to work with in practice.

\begin{definition}[Sampling distribution]
Let $\rx_1, \rx_2, \ldots, \rx_n \stackrel{i.i.d.}{\sim} F$, and let $T: \mathcalX^n \to \real$ be a statistic. The sampling distribution of $T$ under the distribution $F$ is defined as the probability distribution
$$
F_T(t) = \Pr[T(\rx_1, \rx_2, \ldots, \rx_n) \leq t], \quad t \in \real. 
$$
\end{definition}

We always apply statistics to samples, so we often suppress the dependence of the statistic on the sample values. That is, we write simply $T$ instead of $T(\rx_1, \rx_2, \ldots, \rx_n)$.
Using this notation, the sampling distribution becomes $F_T(t) = \Pr[T \leq t]$.
For multi-dimensional parameters $\btheta$, the concepts of sufficiency and sampling distributions are defined analogously.

\begin{exercise}
Let $\rx_1, \rx_2, \ldots, \rx_n \stackrel{i.i.d.}{\sim} \uniformdist(0, \theta)$ (see Exercise~\ref{exercise:uniform_dist}). Show that $T(\rx_1, \rx_2, \ldots, \rx_n) = \max\{\rx_1, \rx_2, \ldots, \rx_n\}$ is a sufficient statistic for $\theta$, and derive its sampling distribution.
\end{exercise}
\begin{exercise}
Let $\rx_1, \rx_2, \ldots, \rx_n \stackrel{i.i.d.}{\sim} \poissondist(\lambda)$ (Definition~\ref{definition:poisson_distribution}). Show that the statistic $T(\rx_1, \rx_2, \ldots, \rx_n) = \sum_{i=1}^{n} \rx_i$ is a sufficient statistic for $\lambda$, and derive its sampling distribution.
\end{exercise}

\index{Unbiased estimator}
\index{Unbiasedness}
\subsection{Unbiasedness under Moment Assumption}\label{section:unbiasedness-moment-assumption}
An introduction to distribution theory can begin by making assumptions only about the moments of the noise---rather than fully specifying its probability distribution---as shown below.
\begin{theoremHigh}[Unbiasedness under moment assumption]\label{theorem:samplding_dist_lse_gaussian_mom}
Let $\rvy = \bX\bbeta + \bepsilon$, where $\bepsilon$ is a random vector of noise. 
We only assume $\Exp[\bepsilon] = \bzero$ and $\Cov[\bepsilon]=\sigma^2\bI$ instead of $\bepsilon \sim \normal(0, \sigma^2 \bI)$. 
Suppose further that $\bX \in \real^{n\times p}$ is fixed and has full rank with $n\geq p$ (i.e., rank is $p$). 
Then, the following moment properties hold:
\begin{enumerate}[(i)]
\item The OLS estimator (i.e., MLE) satisfies $\Exp[\widehatbbeta] = \bbeta$ and $\Cov[\widehatbbeta] = \sigma^2(\bX^\top\bX)^{-1}$.
\item The predicted output satisfies $\Exp[\widehat{\rvy}]=\bX\bbeta$ and  $\Cov[\widehat{\rvy}]=\sigma^2\bH$, where $\bH$ is the orthogonal projection matrix with $\bH = \bX(\bX^\top\bX)^{-1}\bX^\top$ (see Section~\ref{section:ortho_geom_ls}).
\item  The error vector $\rve=\rvy-\widehat{\rvy}$ satisfies $\Exp[\rve]=\bzero$ and  $\Cov[\rve]=\sigma^2(\bI-\bH)$.
\end{enumerate}
\end{theoremHigh}
\begin{proof}[of Theorem~\ref{theorem:samplding_dist_lse_gaussian_mom}]
Since the estimator is $\widehatbbeta=(\bX^\top\bX)^{-1} \bX^\top \rvy$, and we assume that $\rvy = \bX\bbeta +\bepsilon$, we have
\begin{equation}
\begin{aligned}
\widehatbbeta &= (\bX^\top\bX)^{-1} \bX^\top(\bX\bbeta+\bepsilon)
=\bbeta + (\bX^\top\bX)^{-1}\bX^\top\bepsilon. \nonumber
\end{aligned}
\end{equation}
Thus, since $\bX$ is fixed, we obtain that
\begin{equation}
\begin{aligned}
\Exp[\widehatbbeta] &= \bbeta+ (\bX^\top\bX)^{-1}\bX^\top \Exp[\bepsilon]
= \bbeta,  \nonumber
\end{aligned}
\end{equation}
and
\begin{equation}
\begin{aligned}
\Cov[\widehatbbeta]
& = \Cov[(\bX^\top\bX)^{-1} \bX^\top \rvy]
\stackrel{\dag}{=} (\bX^\top\bX)^{-1} \bX^\top \Cov[\rvy] \bX(\bX^\top\bX)^{-1}   \\
&\stackrel{\ddag}{=}(\bX^\top\bX)^{-1} \bX^\top \Cov[\bepsilon] \bX(\bX^\top\bX)^{-1}  
%=(\bX^\top\bX)^{-1} \bX^\top (\sigma^2 \bI) \bX(\bX^\top\bX)^{-1}\\
=\sigma^2(\bX^\top\bX)^{-1},
\end{aligned}
\end{equation}
where the equality ($\dag$) follows from the fact that $\Cov[\bA\rvv+\bb] = \bA\Cov[\rvv]\bA^\top$, and the equality ($\ddag$) follows from the fact $\bX$ is fixed.
And since $\widehat{\rvy} = \bX \widehatbbeta$, we obtain $\Exp[\widehat{\rvy}] = \bX \bbeta$ and $\Cov[\widehat{\rvy}]=\bX \Cov[\widehatbbeta] \bX^\top = \sigma^2\bH$. \footnote{Given non-random matrix $\bA$ and vector $\bb$, we have $\Exp[\bA \rvv +\bb] = \bA \Exp[\rvv] + \bb$ and $\Cov[\bA\rvv+\bb] = \bA\Cov[\rvv]\bA^\top$.}
Furthermore, since $\rve = \rvy -\widehat{\rvy}$, we have $\Exp[\rve] = \Exp[\rvy-\widehat{\rvy}] = \bzero$ and 
$$
\begin{aligned}
\Cov[\rve] 
&= \Cov[\rvy - \widehat{\rvy}] 
\stackrel{\dag}{=} \Cov[\rvy - \bH\rvy]  
= (\bI-\bH)\Cov[\rvy](\bI-\bH)^\top \\
&=(\bI - \bH)\Cov[\bepsilon](\bI-\bH)^\top 
= \sigma^2(\bI - \bH)(\bI - \bH)^\top 
\stackrel{\ddag}{=}\sigma^2(\bI - \bH), 
\end{aligned}
$$
where the equality ($\dag$) follows from the fact that $\widehat{\rvy}=\bH\rvy$, and the equality ($\ddag$) follows from the fact that $(\bI - \bH)$ is an orthogonal projection.
This completes the proof.
\end{proof}

%From the lemma above, we see that the MLE of $\bbeta$ is an unbiased estimator and the expectation of error vectors approaches  zero. This explains why the method of maximum likelihood is so important in the field of point estimation.

%The lemma establishes that the maximum likelihood estimator of $\bbeta$ is an unbiased estimator, and the expected error vectors tend towards zero. This underscores the significance of the maximum likelihood method in the realm of point estimation.
Once again, from this lemma, we see that the maximum likelihood estimator of $\bbeta$ is unbiased, and the expected value of the error vector is zero. This highlights why the method of maximum likelihood plays a central role in point estimation.

\subsection*{Computing the Covariance Matrix}

In least squares problems, computing the associated covariance matrix is essential for assessing the accuracy of the estimated parameters. Specifically, as shown above, the variance of each estimated coefficient $\beta_i$ is proportional to the $i$-th diagonal element of $(\bX^\top\bX)^{-1}$. 
Consider the full-rank linear model: $\by =\bX\bbeta  + \bepsilon$, where $\bepsilon\sim \normal(\bzero, \sigma^2 \bI)$, $\bX \in \real^{n\times p}$, and  $\bepsilon$ is a random vector with zero mean and covariance matrix $\sigma^2\bI$.
The covariance matrix of the least squares estimate $\widehatbbeta$ is $\sigma^2 \bOmega_{\bbeta}$  (Theorem~\ref{theorem:samplding_dist_lse_gaussian_mom}), where
\begin{equation}
\bOmega_{\bbeta} \triangleq (\bX^\top\bX)^{-1} = (\bR^\top \bR)^{-1} = \bR^{-1}\bR^{-\top},
\end{equation}
and $\bR^\top\bR$ is the Cholesky decomposition of $\bX$. 
Alternatively, $\bX=[\bQ_1,\bQ_2]\begin{bmatrixfoot}
\bR\\
\bzero 
\end{bmatrixfoot}=\bQ_1\bR$ can be seen as the QR decomposition of $\bX$.

The inverse $\bS = \bR^{-1} = [s_{ij}]$ is also upper triangular and can be computed in $p^3/3$ flops from the matrix equation $\bR \bS = \bI$ as follows:
\begin{align*}
&\text{for } j = p, p-1, \ldots, 1 \\
&\qquad s_{jj} = 1/r_{jj}; \\
&\qquad \text{for } i = j-1, \ldots, 2, 1 \\
&\qquad\qquad s_{ij} = -\left(\sum_{k=i+1}^{j} r_{ik}s_{kj}\right)/r_{ii}; \\
&\qquad \text{end} \\
&\text{end}
\end{align*}

The computed elements of $\bS$ can overwrite the corresponding elements of $\bR$ in storage. Forming the upper triangular part of $\bOmega_{\bbeta} = \bS\bS^\top $ requires an additional $p^3/3$ flops. This computation can be sequenced so that the elements of $\bOmega_{\bbeta}$ overwrite those of $\bS$ directly. 
The variance of the components of $\widehatbbeta$ is given by the diagonal elements of $\bOmega_{\bbeta}=[\omega_{ij}]$:
\begin{equation}
\omega_{pp} = s_{pp}^2 = 1/r_{pp}^2, \qquad \omega_{ii} = \sum_{j=i}^{p} s_{ij}^2, \quad i = p-1, \ldots, 1.
\end{equation}
Note that the variance for $\beta_p$ is immediately  available  from the last diagonal element $r_{pp}$.

Frequently, $\bOmega_{\bbeta}$ appears only as an intermediate step, such as when computing the variance of a linear functional $\bg^\top \widehatbbeta$. This variance is expressed as:
\begin{equation}
\bg^\top \bOmega_{\bbeta}\bg = \bg^\top \bR^{-1}\bR^{-\top}\bg = \bz^\top \bz,
\end{equation}
where $\bz \triangleq \bR^{-\top}\bg$. 
Thus, instead of evaluating $\bg^\top \bOmega_{\bbeta}\bg$ directly, it is more numerically stable and efficient to solve the lower triangular system $\bR^\top \bz = \bg$, and then compute $\bz^\top \bz$.

Similarly, the covariance matrix of the residual vector $\be = \by - \bX\widehatbbeta$ is calculated as
\begin{equation}
\sigma^2 \bOmega_{\be} \triangleq \sigma^2(\bI - \bX(\bX^\top\bX)^{-1}\bX^\top ) = \sigma^2(\bI - \bQ_1\bQ_1^\top ), \quad  \bX= \bQ_1\bR.
\end{equation}
Note that  $\bI - \bQ_1\bQ_1^\top $ is the orthogonal projector onto the null space of $\bX^\top $ (Remark~\ref{remark:svd_qr_projs}).

\index{Sampling distribution}
\index{Gaussian disturbance}
\subsection{Sampling Distribution of OLS under Gaussian Disturbance}

In addition to understanding the moments of $\widehatbbeta$, can we also gain insight into its distribution? Analogous to the way Gaussian disturbances influence the likelihood function, they also determine the exact distribution of the associated random variables.
\begin{theoremHigh}[Sampling distribution of LS under Gaussian disturbance]\label{theorem:samplding_dist_lse_gaussian}	 
Let $\rvy = \bX\bbeta + \bepsilon$, where $\bepsilon \sim \normal(\bzero, \sigma^2 \bI)$. Assume $\bX$ is fixed and has full rank with $n\geq p$ (i.e., its rank is $p$). Then,
\begin{enumerate}[(i)]
\item  The OLS estimator satisfies $\widehatbbeta \sim \normal(\bbeta, \sigma^2(\bX^\top \bX)^{-1})$;
\item  The predicted output satisfies $\widehat{\rvy} \sim \normal(\bX\bbeta, \sigma^2\bH)$;
\item  The error vector satisfies $\rve = \rvy -\widehat{\rvy} \sim \normal(\bzero, \sigma^2(\bI-\bH))$.
\end{enumerate}
\end{theoremHigh}

\begin{proof}[of Theorem~\ref{theorem:samplding_dist_lse_gaussian}]
Since $\widehatbbeta=(\bX^\top\bX)^{-1} \bX^\top \rvy$ and we assume that $\rvy = \bX\bbeta +\bepsilon$, we have
\begin{equation}
\widehatbbeta = (\bX^\top\bX)^{-1} \bX^\top(\bX\bbeta+\bepsilon)=\bbeta + (\bX^\top\bX)^{-1}\bX^\top\bepsilon. \nonumber
\end{equation}
Since both $\bbeta$ and $(\bX^\top\bX)^{-1}\bX^\top$ are deterministic, we can apply the affine transformation property of the multivariate normal distribution (see Lemma~\ref{lemma:affine_mult_gauss}), giving:
\begin{equation}
\begin{aligned}
\bbeta + (\bX^\top\bX)^{-1}&\bX^\top\bepsilon 
\sim \normal\left(\bbeta, (\bX^\top\bX)^{-1}\bX^\top (\sigma^2\bI)\left[(\bX^\top\bX)^{-1}\bX^\top\right]^\top\right). \nonumber
\end{aligned}
\end{equation}
Thus, it follows that $\widehatbbeta \sim \normal(\bbeta, \sigma^2(\bX^\top \bX)^{-1})$.
Similarly, for $\widehat{\rvy} = \bX \widehatbbeta$, we obtain $\widehat{\rvy} \sim \normal(\bX\bbeta, \sigma^2\bH)$.
Considering the error variable $\rve= \rvy -\widehat{\rvy}$, we have $\rvy\sim \normal(\bX\bbeta, \sigma^2\bI)$, then% and $-\widehat{\rvy} \sim \normal(-\bX\bbeta, \sigma^2\bH)$, Thus,
$$
\begin{aligned}
\be &= \rvy - \widehat{\rvy}
=(\bI-\bH)\rvy \\
&\sim \normal\left((\bI-\bH)\bX\bbeta, (\bI-\bH)(\sigma^2\bI)(\bI-\bH)^\top\right) 	\\
&=  \normal\left(\bzero, \sigma^2(\bI-\bH)\right),
\end{aligned}
$$
where the last equality comes from the fact that $(\bI-\bH)$ is an orthogonal projection onto the perpendicular space of $\cspace(\bX)$, and this completes the proof.
\end{proof}

We further show that, under the same assumptions, the error variable $\rve$ is independent of the predicted output and the OLS estimator.
\begin{lemma}[Independence in LS under Gaussian disturbance]\label{theorem:independence_samplding_dist_lse_gaussian}	 
Let $\rvy = \bX\bbeta + \bepsilon$, where $\bepsilon \sim \normal(\bzero, \sigma^2 \bI)$. Assume $\bX\in \real^{n\times p}$ is fixed and has full rank with $n\geq p $ (i.e., its rank is $p$). Then, the random variable $\rve$ is independent of $\widehat{\rvy}$ and $\widehatbbeta$;
\end{lemma}	
\index{Gaussian disturbance}

\begin{proof}[of Lemma~\ref{theorem:independence_samplding_dist_lse_gaussian}]
It is straightforward that 
\begin{equation}
\begin{aligned}
\Cov[\rve, \widehat{\rvy}] &= \Cov[(\bI-\bH)\rvy, \bH\rvy]
= (\bI-\bH)\Cov[\rvy, \rvy]\bH^\top  \\ 
&= (\bI-\bH)(\sigma^2\bI)\bH 
= \sigma^2(\bH-\bH) 
= \bzero \nonumber,
\end{aligned}
\end{equation}
where the second equality follows from the fact that $\Cov[\bA\rvv, \bB\rvw] = \bA\Cov[\rvv,\rvw]\bB^\top$, given  non-random matrix $\bA$ and $\bB$.
As $\widehat{\rvy} = \bX\widehatbbeta$, and $\bX$ is the observed data matrix and is fixed, $\rve$ is independent of $\widehatbbeta$ as well.
This completes the proof.
\end{proof}

Using these results, we can show that $\bH\rvy$ is a sufficient statistic (Definition~\ref{definition:sufficiency}) for the parameter $\bbeta$.
Write $\rvy = \bH\rvy + (\bI - \bH)\rvy = \widehat{\rvy} + \rve$.
Now define the $2n$-dimensional vector $\rvz \triangleq [\widehat{\rvy}^\top, \rve^\top]^\top$.
%If we can show that the conditional distribution of the $2n$-dimensional vector $\rvz \triangleq [\widehat{\rvy}^\top, \rve^\top]^\top$ given $\widehat{\rvy}$ does not depend on $\bbeta$, then we will also know that the conditional distribution of $\rvy = \widehat{\rvy} + \rve$ given $\widehat{\rvy}$ does not depend on $\bbeta$ either, proving the sufficiency.
Since $\widehat{\rvy}$ is independent of $\rve$ (Lemma~\ref{theorem:independence_samplding_dist_lse_gaussian}), conditional on $\widehat{\rvy}$, $\rve$ always has the same distribution $\mathcalN(\bzero, \sigma^2(\bI - \bH))$. It follows that, conditional on $\widehat{\rvy}$, the vector $\rvz$ has a distribution whose first $n$ coordinates equal $\widehat{\rvy}$ almost surely, and whose last $n$ coordinates are $\mathcalN(\bzero, \sigma^2(\bI - \bH))$. Neither of those two depend on $\bbeta$,  proving its sufficiency.

%TODO: The independence helps with $t$-tests, where the numerator and denominator must be independent.

\subsection{$t$-Distribution and $F$-Distribution}
To discuss further results, we need the definitions of $t$- and $F$-distributions.
We have introduced that Chi-squared distribution is a specific  case of the Gamma distribution (Definitions~\ref{definition:gamma_distri} and \ref{definition:chisquare_dist}). 
In this context, we provide the formal definitions of the $t$-distribution  and the $F$-distribution, with a particular emphasis on their close relationship to the Chi-squared distribution.

\begin{figure}[h!]
\centering                      
\vspace{-0.35cm}                  
\subfigtopskip=2pt              
\subfigbottomskip=2pt            
\subfigcapskip=-5pt            
\subfigure[$t$-distribution probability density
functions for different values of the parameter $n$. When increasing the degree of freedom $n$, the density gets closer to $\normal(0,1)$.]{\label{fig:dists_tdistribution}
\includegraphics[width=0.48\linewidth]{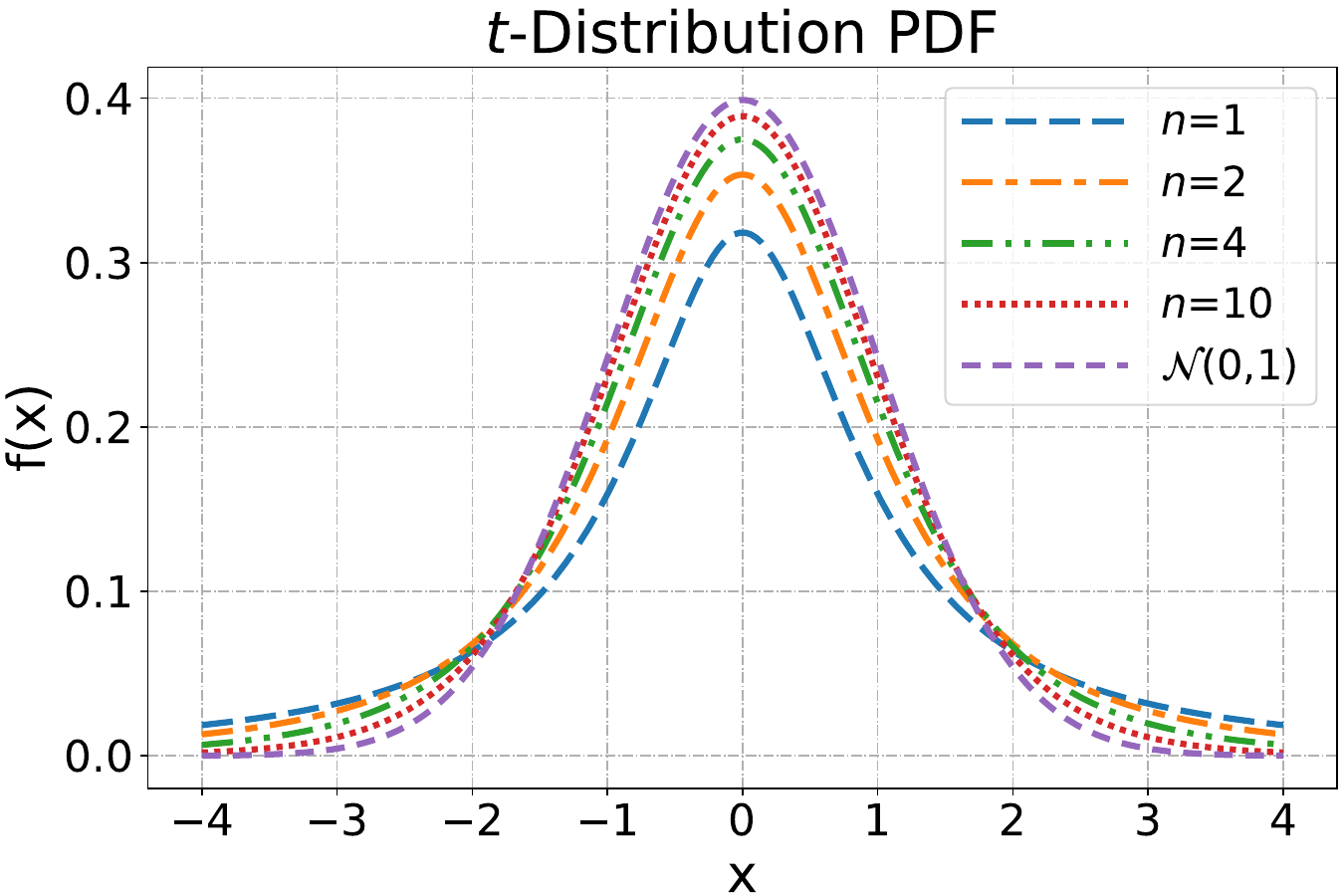}}
\subfigure[$F$-distribution probability density
functions for different values of the parameters $n$ and $d$.]{\label{fig:dists_fdistribution}
\includegraphics[width=0.48\linewidth]{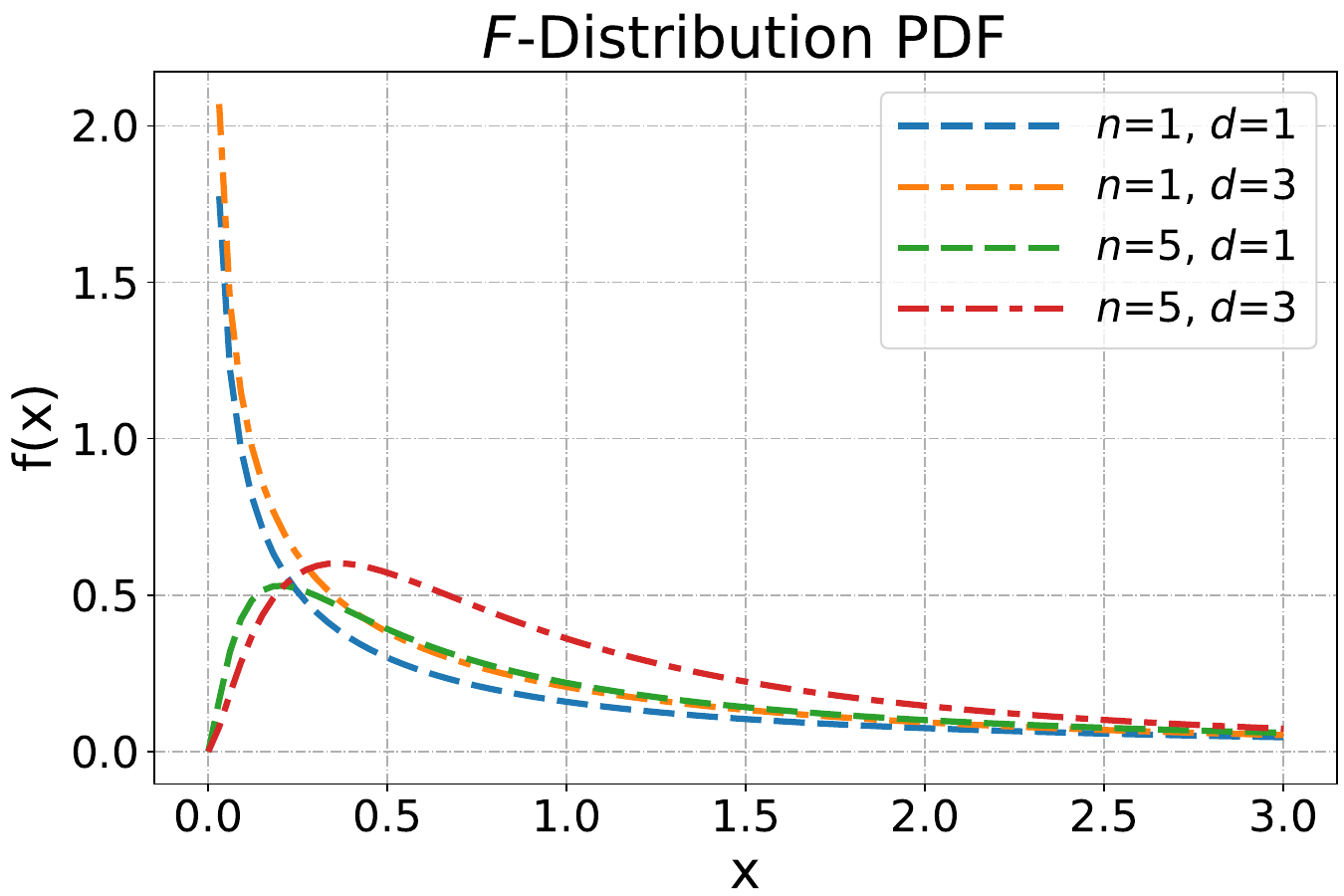}}
\caption{$t$-distribution and $F$-distribution probability density functions.}
\label{fig:caaa}
\end{figure}

\begin{definition}[$t$-distribution]\index{$t$-distribution}
Let $\ry$ and $\rz$ be independent random variables such that $\ry\sim \normal (0,1)$ and $\rz\sim \chi_{(n)}^2$. 
Then the random variable defined by $\rx\triangleq \frac{\ry}{\sqrt{\rz/n}}$ follows a $t$-distribution with $n$ degrees of freedom, denoted  $\rx\sim t_{(n)}$. The probability density function of $\sqrt{\rz/n}$ is given by 
$$ h_2(z; n)=\left\{
\begin{aligned}
&\frac{2n^{n/2}}{2^{n/2}\Gamma(\frac{n}{2})} z^{n-1} \exp(-\frac{nz^2}{2}) ,& \mathrm{\,\,if\,\,} z \geq 0;  \\
&0 , &\mathrm{\,\,if\,\,} z <0.
\end{aligned}
\right.
$$
And the probability density function of $\ry$ is given by 
$$
h_1(y) = \frac{1}{\sqrt{2\pi}} \exp\{-\frac{y^2}{2}\}.
$$
Then the probability density function of the $t$-distribution can be obtained by 
$$
\begin{aligned}
g(x;n)&=\int_{0}^{\infty} t \cdot h_1(x t) \cdot  h_2(t) dt 
= \frac{\Gamma(\frac{n+1}{2})}{\sqrt{n\pi} \Gamma(\frac{n}{2})} \left(1+\frac{x^2}{n}\right)^{-\frac{n+1}{2}}.
\end{aligned}
$$
The shape of the $t$-distribution resembles that of the standard normal distribution---it is symmetric about zero.  
However, it has heavier tails when $n$ is small.
And when $n\rightarrow \infty$, $t$-distribution converges to $\normal(0,1)$.
Figure~\ref{fig:dists_tdistribution} compares  the $t$-distribution for various values of $n$ and $\normal(0,1)$.
\end{definition}
Note that $t$-distribution is the univariate version of Definition~\ref{definition:multivariate-stu-t}.

\begin{definition}[$F$-distribution]\index{$F$-distribution}
The $F$-distribution arises as the ratio of two independent Chi-squared distributions, each divided by its degrees of freedom. 
Using $n$ and $d$ to denote numerator and denominator, respectively, we define 
$$
\begin{aligned}
\frac{\frac{1}{n}\chi_{(n)}^2}{\frac{1}{d}\chi_{(d)}^2} \sim F_{n,d}.
\end{aligned}
$$
And the probability density function is given by 
$$ f(x; n,d)=\left\{
\begin{aligned}
&\frac{\Gamma(\frac{n+d}{2})}{\Gamma(\frac{n}{2}) \Gamma(\frac{d}{2})}
n^{n/2} d^{d/2} x^{\frac{n}{2}-1} \left(d+nx\right)^{-\frac{1}{2}(n+d)}
,& \mathrm{\,\,if\,\,} x \geq 0;  \\
&0 , &\mathrm{\,\,if\,\,} x <0.
\end{aligned}
\right.
$$
A notable special case is that if $\rx\sim t_{(n)}$, then $\rx^2\sim F_{1,n}$.
Figure~\ref{fig:dists_fdistribution} illustrates how the shape of the $F$-distribution varies with different combinations of $n$ and $d$.
\end{definition}

\index{Sampling distribution}
\index{Sesidual sum of squares}
\subsection{Sampling Distribution of RSS under Gaussian Disturbance}\label{sec:dist_sse}

Another crucial result we will prove is the sum of squares due to error (or called the \textit{residual sum of squares (RSS)}) $\rve^\top\rve \sim\sigma^2 \chi^2_{(n-p)} $, by which we can construct  an unbiased estimator of $\sigma^2$. The distribution $\chi^2_{(n-p)}$ is the Chi-squared distribution with  $n-p$ degrees of freedom (Definition~\ref{definition:chisquare_dist}).

\begin{theoremHigh}[Distribution of sum of squares due to error]\label{theorem:ss-chisquare}
Let $\rvy = \bX\bbeta + \bepsilon$, where $\bepsilon \sim \normal(\bzero, \sigma^2 \bI)$. Assume $\bX\in\real^{n\times p}$ is fixed and has full rank with $n\geq p$ (i.e., its rank is $p$). Then, we have 
\begin{enumerate}[(i)]
\item  Sum of squares due to error satisfies $\rve^\top\rve=\sum_i^n \re_i^2 \sim \sigma^2 \chi^2_{(n-p)}$, where $\rve = \rvy - \widehat{\rvy}$;

\item  An unbiased estimator of $\sigma^2$ is $\rS^2=\frac{1}{n-p} (\rvy-\widehat{\rvy})^\top(\rvy-\widehat{\rvy})=\frac{\rve^\top\rve}{n-p}$;
\item  The random variables $\widehatbbeta$ and $\rS^2$ are independent.
\end{enumerate}
\end{theoremHigh}
\index{Spectral decomposition}
\begin{proof}[of Theorem \ref{theorem:ss-chisquare}]
We realize that $\sum_i^n \re_i^2 = \rve^\top\rve = [(\bI-\bH)\rvy]^\top[(\bI-\bH)\rvy]$. By expressing $\rvy$ into $\rvy =\bX\bbeta+\bepsilon$ and using the fact that $\bH\bX\bbeta = \bX\bbeta$, we have:
\begin{equation}
\begin{aligned}
(\bI-\bH)\rvy &= (\bI - \bH)(\bX\bbeta+\bepsilon) 
= (\bI - \bH)\bepsilon, \nonumber
\end{aligned}
\end{equation} 
by which we can rewrite $[(\bI-\bH)\rvy]^\top[(\bI-\bH)\rvy]$ as 
\begin{equation}
\begin{aligned}
\sum_i^n \re_i^2 &= [(\bI-\bH)\rvy]^\top[(\bI-\bH)\rvy] =\bepsilon (\bI-\bH)^\top(\bI-\bH)\bepsilon 
= \bepsilon^\top(\bI-\bH)\bepsilon. \nonumber
\end{aligned}
\end{equation}
By Spectral Theorem~\ref{theorem:spectral_theorem} and Proposition~\ref{proposition:eigenvalues-of-projection} (the only possible eigenvalues of the hat matrix are 0 and 1), we can express the sum of squares due to error as $\sum_{i=1}^n \re_i^2 = \bepsilon^\top(\bI-\bH)\bepsilon = \bepsilon^\top(\bQ \bLambda\bQ^\top)\bepsilon$, where $\bI-\bH=\bQ \bLambda\bQ^\top$ is the spectral decomposition of $\bI-\bH$. 
Given the fact that rotations on the normal distribution do not affect the distribution~(Lemma~\ref{lemma:rotat_multi_gauss}), we can define
\begin{equation}
\boldeta \triangleq \bQ^\top\bepsilon \sim \normal(\bzero, \sigma^2\bI). \nonumber
\end{equation}
Thus, it follows that 
\begin{equation}
\sum_{i=1}^n \re_i^2 = \boldeta^\top \bLambda \boldeta \sim \sigma^2 \chi^2_{\mathrm{rank(\bI-\bH)}}\sim\sigma^2 \chi^2_{(n-p)}, \nonumber
\end{equation}
where, according to Lemma~\ref{lemma:rank-of-symmetric-idempotent},
$$
\begin{aligned}
\rank(\bI-\bH) 
&= \trace(\bI) - \trace(\bH)
=n-\trace(\bX (\bX^\top\bX)^{-1}\bX^\top)\\
&=n-\trace( (\bX^\top\bX)^{-1}\bX^\top\bX) 
=n-p.
\end{aligned}
$$

Therefore, it can be shown that  $\Exp[\rve^\top\rve] = \sigma^2 (n-p)$, leading to an  unbiased estimator for $\sigma^2$, denoted by $\rS^2=\frac{\rve^\top\rve}{n-p}$.

As proved in Lemma~\ref{theorem:independence_samplding_dist_lse_gaussian}, $\rve$ is independent of $\widehat{\rvy}$.
Then, $\rS^2=\frac{\rve^\top\rve}{n-p}$ is independent of $\widehatbbeta$ as well.
This completes the proof.
\end{proof}

\index{GLS}
\begin{remark}[Sampling distribution for GLS]\label{remark:samp_gls}
In the generalized least squares (GLS) setting described in Theorem~\ref{theorem:gauss_markov_gls}, the covariance matrix of the estimator $\widehatbbeta$ is given by
\begin{equation}
\Cov[\widehatbbeta] = \sigma^2 (\bX^\top \bOmega^{-1} \bX)^{-1} \in \real^{n \times n},
\end{equation}
and an unbiased estimate of $ \sigma^2 $ is
\begin{equation}
S^2 = \frac{1}{n-p} \be^\top \bOmega^{-1} \be , \quad \be \triangleq \by - \bX \widehatbbeta.
\end{equation}
\end{remark}

\paragrapharrow{Degree of freedom.}
To see why $\rS^2$ (which is the sum of squares due to error divided by $(n-p)$) is an unbiased estimator of $\sigma^2$, we delve into the concept of  degrees of freedom.
\index{Degree of freedom}
\begin{remark}[Degrees of freedom of the error vector $\rve$]\label{remark:df-of-error-vector}
The unbiased estimator $\rS^2$ of $\sigma^2$ adjusts the \textit{degree of freedom (df)} of $\rve$. 
Specifically, if $\bX \in \real^{n\times p}$ has full column rank with $n\geq p$, then  the degrees of freedom associated with $\rve$ is $(n-p)$.
\end{remark}
It may not be immediately clear why $(n-p)$   is referred to as the degrees of freedom of $\rve$. In general, degrees of freedom represent the dimension of the space in which a vector can vary---that is, how freely it can move within that space.
Since $\rve\in \real^n$ lies orthogonal to the column space of $\bX$ as shown in Figure~\ref{fig:ls-geometric2}. That is, $\bX^\top \rve = \bzero$ (see Lemma~\ref{lemma:opt_cond_ls}), and $\rve$ is in the null space of $\bX^\top$, which has dimension $n-p$.
Thus, although $\rve$ resides in $\real^n$, it is constrained by $p$ linear relationships and therefore loses $p$ degrees of freedom.

With this unbiased estimator of noise variance $\sigma^2$, we are now equipped to answer various inferential questions. Here is a typical example:
\begin{example}[Application of sampling distribution]
Given any non-random vector $\bc$, we wish to find the distribution for the following equation:
\begin{equation}
\frac{\bc^\top\widehatbbeta - \bc^\top\bbeta}{\rS \sqrt{\bc^\top(\bX^\top\bX)^{-1}\bc}}. \nonumber
\end{equation}
From $\widehatbbeta\sim \normal(\bbeta, \sigma^2(\bX^\top\bX)^{-1})$, we can find $(\bc^\top\widehatbbeta - \bc^\top\bbeta) \sim \normal(0, \sigma^2\bc^\top(\bX^\top\bX)^{-1}\bc)$. This makes
$$
\frac{\bc^\top\widehatbbeta - \bc^\top\bbeta}{\sigma \sqrt{\bc^\top(\bX^\top\bX)^{-1}\bc}}\sim \normal(0, 1).
$$
Recall that $\rS^2 \sim \frac{\sigma^2 \chi^2_{(n-p)}}{n-p}$. We have $\frac{\rS^2}{\sigma^2} \sim \frac{ \chi^2_{(n-p)}}{n-p} $.
This implies 
\begin{equation}
\frac{\bc^\top\widehatbbeta - \bc^\top\bbeta}{\rS \sqrt{\bc^\top(\bX^\top\bX)^{-1}\bc}}\sim t_{(n-p)}, \nonumber
\end{equation}
which follows a $t$-distribution (suppose $\rx\sim \normal(0,1)$ and $\ry\sim \chi^2_{(n)}$; then, $\frac{\rx}{\sqrt{\ry/n}} \sim t_{(n)}$), and by which we could answer question of interest; for example, this result enables us to construct confidence intervals and perform hypothesis tests on individual coefficients.

For example, let $\bc=\be_k$ be the $k$-th unit basis vector. Then the confidence interval for the  $k$-th coordinate is 
$$
\be_k^\top\widehatbbeta \pm t_{(n-p)}(\alpha/2){\rS \sqrt{\be_k^\top(\bX^\top\bX)^{-1}\be_k}}.
$$

When $\bc=\bx_{\new}$ is a new data observation, and we want to predict $y_{\new}=\bx_{\new}^\top\widehatbbeta$. Then the $(1-\alpha)\times 100\%$ confidence interval is 
$$
\bx_{\new}^\top\widehatbbeta \pm t_{(n-p)}(\alpha/2){\rS \sqrt{\bx_{\new}^\top(\bX^\top\bX)^{-1}\bx_{\new}}}.
$$
In many cases, we may model $y_{\new}=\bx_{\new}^\top\bbeta +\epsilon$ where $\epsilon\sim\normal(0, \sigma^2)$. Then the confidence interval becomes 
$$
\bx_{\new}^\top\widehatbbeta \pm t_{(n-p)}(\alpha/2){\rS \sqrt{1+\bx_{\new}^\top(\bX^\top\bX)^{-1}\bx_{\new}}}.
$$
This is known as the \textit{prediction interval}, which provides confidence bounds for a future observed response rather than for the expected value alone.
\end{example}

\paragrapharrow{Minimum MSE estimator of noise variance $\sigma^2$.}
Moreover, we decompose the noise variance in terms of mean squared error, from which we could find the minimum MSE estimator of the noise variance. 
\index{Bias-variance decomposition}

\begin{lemma}[Bias-variance decomposition of noise variance]\label{lemma:bias-variance-noise-variance}
For any estimator $\bar{\sigma}^2$ of $\sigma^2$, the mean squared error of the estimator have the following decomposition:
\begin{equation}
\begin{aligned}
\MSE(\bar{\sigma}^2, \sigma^2) 
&=\Exp[(\bar{\sigma}^2 - \sigma^2)^2] = \Exp[(\bar{\sigma}^2)^2] - 2\sigma^2\Exp[\bar{\sigma}^2] +(\sigma^2)^2\\
&= \normtwo{\Exp[\bar{\sigma}^2] - \sigma^2}^2 + \Exp\big[\normtwo{\bar{\sigma}^2 - \Exp(\bar{\sigma}^2)}^2\big] \\
&= \normtwo{\bias(\bar{\sigma}^2, \sigma^2)}^2 + \Var[\bar{\sigma}^2].\\\nonumber
\end{aligned}
\end{equation}
\end{lemma}

\index{Minimum MSE estimator}
In Theorem~\ref{theorem:mle-gaussian}, we showed that the MLE of $\sigma^2$ is $\widehat{\sigma}^2=\frac{1}{n}\rve^\top\rve $, which is a biased estimator of $\sigma^2$.
According to Theorem~\ref{theorem:ss-chisquare}, an unbiased estimator of $\sigma^2$ is given by $\rS^2=\frac{1}{n-p}\rve^\top\rve$.

Define the function $\bar{\sigma}^2(k) = \frac{1}{k}\rve^\top\rve$. Then, the MLE of $\sigma^2$ can be denoted by $\bar{\sigma}^2(n)$, and the unbiased estimator $\rS^2$ of $\sigma^2$ can be expressed as  $\rS^2=\bar{\sigma}^2(n-p)$. 
The value of $k$ that minimizes the mean squared error $\MSE(\bar{\sigma}^2(k), \sigma^2)$ is $k=n-p+2$. 
We thus have 
$$
\begin{aligned}
\MSE(\widehat{\sigma}^2, \sigma^2) &= (\frac{n-p}{n}\sigma^2-\sigma^2)^2 + \frac{2(n-p)\sigma^4}{{n}^2} = \frac{p^2+2(n-p)}{n^2}\sigma^4;\\
\MSE(\rS^2, \sigma^2) &= (\frac{n-p}{n-p}\sigma^2-\sigma^2)^2 + \frac{2(n-p)\sigma^4}{{(n-p)}^2} = \frac{2\sigma^4}{n-p}; \\
\MSE(\bar{\sigma}^2(n-p+2), \sigma^2) &= (\frac{n-p}{n-p+2}\sigma^2-\sigma^2)^2 + \frac{2(n-p)\sigma^4}{{(n-p+2)}^2} = \frac{2n-2p}{(n-p+2)^2}\sigma^4,
\end{aligned}
$$
which implies
$$
\MSE(\bar{\sigma}^2(n-p+2), \sigma^2) \leq \MSE(\rS^2, \sigma^2).
$$
The result shows that $\bar{\sigma}^2(n-p+2)$ has a smaller mean squared error than $\rS^2$, meaning it tends to be closer to $\sigma^2$ under this performance measure.
However, $\bar{\sigma}^2(n-p+2)$ is biased and will underestimate $\sigma^2$ on average \footnote{When the bias at some coordinate of $\sigma^2$ is positive, we call it \textit{overestimation}; when it is negative, we call it \textit{underestimation}.}. This bias raises concerns about the reliability of $\bar{\sigma}^2(n-p+2)$ as a general-purpose estimator for $\sigma^2$.

In general, since MSE depends on the true parameter value, there is no single estimator that uniformly minimizes MSE across all possible values $\sigma^2$. Therefore, we often restrict our attention to a specific class of estimators---such as unbiased estimators---and seek the one with the lowest variance.
A commonly used approach is to focus on unbiased estimators and choose the one with minimum variance; such an estimator is called the \textit{best unbiased estimator (BUE)}. If we further restrict our attention to linear estimators, we obtain the \textit{best linear unbiased estimator (BLUE)}.
See Section~\ref{section:beta-blue}  for a discussion about the BLUE of $\bbeta$. 

%In general, since MSE is a function of the parameter, there will not be just one ``best" estimator in terms of MSE. 
%Thus, we should restrict ourselves to finding the ``best" estimator in a limited class of estimators. 
%A popular way involves considering only unbiased estimators and selecting the one with the lowest variance. That is, the \textit{best unbiased estimator (BUE)}. 
%Further restrictions to linear estimators lead to the \textit{best linear unbiased estimator (BLUE)}. 

%We should also notice that for any unbiased estimator $\sigma^2_1$ and $\sigma_2^2$ of $\sigma^2$, the linear combination $\sigma_3^2 = a\sigma^2_1 + (1-a)\sigma^2_2$ also qualifies as an unbiased estimator for $\sigma^2$.

\index{BUE}
\index{Best unbiased estimator}
\index{Best linear unbiased estimator}
\index{Unbiased estimator}

In summary, various estimators for $\bbeta$ and $\sigma^2$ are compared in Table~\ref{table:diff-estitors}.
\begin{table}[h!]\centering
\setlength{\tabcolsep}{5.3pt}    
\begin{tabular}{c|c|c|c|c}
\hline
& OLS Estimator                   & MLE                             & \multicolumn{1}{l|}{Unbiased Estimator} & \multicolumn{1}{l}{Minimum MSE Estimator} \\ \hline\hline
$\bbeta$   & $(\bX^\top\bX)^{-1}\bX^\top\rvy$ & $(\bX^\top\bX)^{-1}\bX^\top\rvy$ & $(\bX^\top\bX)^{-1}\bX^\top\rvy$         &                                           \\ \hline
$\sigma^2$ &                                 & $\frac{1}{n}\rve^\top\rve$        & $\frac{1}{n-p}\rve^\top\rve$              & $\frac{1}{n-p+2}\rve^\top\rve$              \\ \hline
\end{tabular}\caption{Comparison of different estimators for $\bbeta$ and $\sigma^2$.}\label{table:diff-estitors}
\end{table}

\index{In-sample error}
\index{Out-of-sample evaluation}
\index{Learning curve}
\subsection{Learning Curve of Least Squares under Gaussian Disturbance}
To differentiate the test data error (the data we do not see), we introduce the concept of  \textbf{in-sample error} (also known as  the in-sample sum of squares due to error) by $\MSE_{\text{in}}(\bbeta) = \frac{1}{n} \sum_{i=1}^{n} (\ry_i - \widehat{\ry}_i)^2 = \frac{1}{n}\rve^\top\rve$ for $n$ available data samples.
Additionally, we define the \textbf{out-of-sample error} (also known as  the out-of-sample sum of squares due to error) as the expected squared error of test data, given by $\MSE_{\text{out}}(\bbeta) = \Exp[(\ry_\ast - \widehat{\ry}_\ast)^2]$.
We then derive the expressions for the expected in-sample error and out-of-sample error under the assumption of Gaussian noise disturbance.
\begin{theoremHigh}[Expectation of in-sample error under Gaussian disturbance]\label{theorem:in-sample-error}
Let $\rvy = \bX\bbeta + \bepsilon$, where $\bepsilon \sim \normal(\bzero, \sigma^2 \bI)$. Assume $\bX\in \real^{n\times p}$ is fixed and has full rank with $p<n$ (i.e., its rank is $p$). Then, we have 
\begin{enumerate}[(i)]
\item  The expected  in-sample error: $\Exp[\MSE_{\text{in}}(\bbeta)]= \frac{n-p}{n}\sigma^2$;
\item  The expected  out-of-sample error: $\Exp[\MSE_{\text{out}}(\bbeta)]$ converges to $\frac{n+p}{n}\sigma^2+ \mathcalO(\frac{1}{n})$.
\end{enumerate}
\end{theoremHigh}

\begin{proof}[of Theorem~\ref{theorem:in-sample-error}]
As a recap,	the sum of squares due to error is defined as 
$$
\begin{aligned}
\rve^\top\rve 
&=\normtwo{\rvy - \widehat{\rvy}}^2=\normtwo{\bX\bbeta+\bepsilon - \bH\rvy}^2 
=\normtwo{\bX\bbeta+\bepsilon - \bH(\bX\bbeta+\bepsilon)}^2\\
&\stackrel{\dag}{=}\normtwo{\bX\bbeta+\bepsilon - \bX\bbeta-\bH\bepsilon}^2 
=\normtwo{(\bI-\bH)\bepsilon }^2  =\left((\bI-\bH)\bepsilon\right)^\top((\bI-\bH)\bepsilon) \\
&=\bepsilon^\top(\bI-\bH)^\top(\bI-\bH)\bepsilon  
=\bepsilon^\top(\bI-\bH)(\bI-\bH)\bepsilon 
=\bepsilon^\top(\bI-\bH)\bepsilon, 
\end{aligned}
$$
where the equality ($\dag$) follows from the fact that $\bX\bbeta$  lies in $\cspace(\bX)$, and the last two equalities arise from the fact that $\bI-\bH$ is symmetric and idempotent.
Next, by taking  the expectation of the sum of squares due to error with respect to $\bepsilon$, we obtain:
$$
\begin{aligned}
\Exp[\rve^\top\be] &=\Exp[\bepsilon^\top(\bI-\bH)\bepsilon] 
=\trace( (\bI-\bH) \cdot \sigma^2\bI )
=\sigma^2(n-p),    
\end{aligned}
$$
where the second equality follows from the fact that: for random variable $\rvb$ and non-random matrix $\bA$, we have
\begin{equation}\label{equation:trace-to-expectation}
\Exp[\rvb^\top\bA\rvb] = \trace(\bA \Cov[\rvb]) + \Exp[\rvb]^\top\bA\Exp[\rvb];
\end{equation}
see Lemma~\ref{lemma:quad_tra_prob}.
This expectation aligns with the outcome derived in Theorem~\ref{theorem:ss-chisquare}, where we establish that  $\rve^\top\rve\sim \sigma^2 \chi^2_{(n-p)}$ with a  expectation of $\sigma^2(n-p)$. 
Consequently, we have  
$$
\Exp[\MSE_{\text{in}}(\bbeta)]=\frac{1}{n}\Exp[\rve^\top\be] = \frac{n-p}{n}\sigma^2.
$$
Note here, we can directly obtain the expectation of $\rve^\top\rve$ from  Theorem~\ref{theorem:ss-chisquare}. 
The presented proof offers an alternative approach to determine the expectation of $\rve^\top\rve$.

For the second part of the claim, 
given the test input $\rvx_\ast$, test output $\ry_\ast$, and test noise $\epsilon_\ast$,  the test error is
\begin{equation}
\begin{aligned}
\re_\ast 
&=\ry_\ast - \rvx_\ast^\top\widehatbbeta
=\ry_\ast - \rvx_\ast^\top(\bX^\top\bX)^{-1}\bX^\top\rvy 
=(\rvx_\ast^\top\bbeta+\epsilon_\ast) - \rvx_\ast^\top(\bX^\top\bX)^{-1}\bX^\top(\bX\bbeta+\bepsilon) \\
&=\epsilon_\ast - \rvx_\ast^\top(\bX^\top\bX)^{-1}\bX^\top\bepsilon + \left[\rvx_\ast^\top\bbeta-\rvx_\ast^\top(\bX^\top\bX)^{-1}\bX^\top\bX\bbeta\right]
=\epsilon_\ast - \rvx_\ast^\top(\bX^\top\bX)^{-1}\bX^\top\bepsilon. \nonumber
\end{aligned}
\end{equation}
Then, the squared test error can be obtained by
\begin{equation}
\begin{aligned}
\re_\ast^2 &=(\epsilon_\ast - \rvx_\ast^\top(\bX^\top\bX)^{-1}\bX^\top\bepsilon)^2 \\
&=\epsilon_\ast^2-2\epsilon_\ast \rvx_\ast^\top(\bX^\top\bX)^{-1}\bX^\top\bepsilon + \left(\rvx_\ast^\top(\bX^\top\bX)^{-1}\bX^\top\bepsilon\right)\left(\rvx_\ast^\top(\bX^\top\bX)^{-1}\bX^\top\bepsilon\right)^\top \\
&=\epsilon_\ast^2-2\epsilon_\ast \rvx_\ast^\top(\bX^\top\bX)^{-1}\bX^\top\bepsilon + \rvx_\ast^\top( \bX^\top\bX)^{-1}\bX^\top\bepsilon \bepsilon^\top \bX (\bX^\top\bX)^{-1}  \rvx_\ast \nonumber
\end{aligned}
\end{equation}

\paragraph{Step 1.} Taking the expectation of the squared test error with respect to the test input $\rvx_\ast$,
\begin{equation}
\begin{aligned}
\Exp_{\rvx_\ast}[\re_\ast^2] &=\epsilon_\ast^2-2\epsilon_\ast \Exp_{\rvx_\ast}[\rvx_\ast]^\top(\bX^\top\bX)^{-1}\bX^\top\bepsilon + \trace\left(\rmM  \Cov_{\rvx_\ast}[\rvx_\ast]\right) + \Exp_{\rvx_\ast}[\rvx_\ast]^\top\rmM\Exp_{\rvx_\ast}[\rvx_\ast],  \nonumber
\end{aligned}
\end{equation}
where $\rmM\triangleq( \bX^\top\bX)^{-1}\bX^\top\bepsilon \bepsilon^\top \bX (\bX^\top\bX)^{-1}$, and the last two components follows again from Equation~\eqref{equation:trace-to-expectation}.

\paragraph{Step 2.} Taking the expectation of the squared test error with respect to the test noise $\epsilon_\ast$, 
\begin{equation}
\small
\begin{aligned}
\Exp_{\epsilon_\ast}[\Exp_{\rvx_\ast}[\re_\ast^2]] 
&=\Exp_{\epsilon_\ast}[\epsilon_\ast^2]-2\Exp_{\epsilon_\ast}[\epsilon_\ast] \Exp_{\rvx_\ast}[\rvx_\ast]^\top(\bX^\top\bX)^{-1}\bX^\top\bepsilon
+ \trace\left(\rmM  \Cov[\rvx_\ast]\right) + \Exp_{\rvx_\ast}[\rvx_\ast]^\top\rmM\Exp_{\rvx_\ast}[\rvx_\ast] \\
&=\sigma^2 +\trace\left(\rmM  \Cov_{\rvx_\ast}[\rvx_\ast]\right) + \Exp_{\rvx_\ast}[\rvx_\ast]^\top\rmM\Exp_{\rvx_\ast}[\rvx_\ast]   \nonumber
\end{aligned}
\end{equation}

\paragraph{Step 3.} Taking the expectation of the squared test error with respect to the input noise $\bepsilon$, 
\begin{equation} \label{equation:learning-curve-part3}
\begin{aligned}
\Exp_{\bepsilon}[\Exp_{\epsilon_\ast}[\Exp_{\rvx_\ast}[\re_\ast^2]]] &=\sigma^2 +\Exp_{\bepsilon}\left[\trace\left(( \bX^\top\bX)^{-1}\bX^\top\bepsilon \bepsilon^\top \bX (\bX^\top\bX)^{-1}  \Cov_{\rvx_\ast}[\rvx_\ast]\right)\right] \\
&\gap + \Exp_{\bepsilon}\left[\Exp_{\rvx_\ast}[\rvx_\ast]^\top( \bX^\top\bX)^{-1}\bX^\top\bepsilon \bepsilon^\top \bX (\bX^\top\bX)^{-1}\Exp_{\rvx_\ast}[\rvx_\ast]\right]. \\
%&=\sigma^2 +\Exp_{\bepsilon}\left[tr\left( \bepsilon^\top \bX (\bX^\top\bX)^{-1}  \Cov_{\rvx_\ast}[\rvx_\ast] ( \bX^\top\bX)^{-1}\bX^\top\bepsilon\right)\right] \\
%&+ \Exp_{\bepsilon}\left[\bepsilon^\top \bX (\bX^\top\bX)^{-1}\Exp_{\rvx_\ast}[\rvx_\ast] \Exp_{\rvx_\ast}[\rvx_\ast]^\top( \bX^\top\bX)^{-1}\bX^\top\bepsilon \right]\\
%&=\sigma^2 +\Exp_{\bepsilon}\left[tr\left( \bepsilon^\top \bX (\bX^\top\bX)^{-1}  \Cov_{\rvx_\ast}[\rvx_\ast] ( \bX^\top\bX)^{-1}\bX^\top\bepsilon\right)\right] \\
%&+ \Exp_{\bepsilon}\left[\bepsilon^\top \bX (\bX^\top\bX)^{-1}\Exp_{\rvx_\ast}[\rvx_\ast] \Exp_{\rvx_\ast}[\rvx_\ast]^\top( \bX^\top\bX)^{-1}\bX^\top\bepsilon \right]  \nonumber
\end{aligned}
\end{equation}

\paragraph{Step 3.1.} For the second part of the above equation, since the trace of a product is invariant under cyclical permutations of the factors, we have:
$$
\begin{aligned}
&\,\,\,\,\,\,\,\,\Exp_{\bepsilon}\left[\trace\left(( \bX^\top\bX)^{-1}\bX^\top\bepsilon \bepsilon^\top \bX (\bX^\top\bX)^{-1}  \Cov_{\rvx_\ast}[\rvx_\ast]\right)\right] \\
&=\Exp_{\bepsilon}\left[\trace\left( \bepsilon^\top \bX (\bX^\top\bX)^{-1}  \Cov_{\rvx_\ast}[\rvx_\ast] ( \bX^\top\bX)^{-1}\bX^\top\bepsilon\right)\right].
\end{aligned}
$$
%\begin{mdframed}[hidealllines=\mdframehidelineNote,backgroundcolor=\mdframecolor]
Following the fact that: for random variable $\rvb$ and non-random matrix $\bA$, since $\rvb^\top\bA\rvb$ is a scalar, we have 
$
\rvb^\top\bA\rvb = 	\trace(\rvb^\top\bA\rvb) = \trace(\bA\rvb\rvb^\top), 
$
where the last equation follows from the fact that the trace of a product is invariant under cyclical permutations of the factors:
$$
\trace(\bA\bB\bC) = \trace(\bB\bC\bA) = \trace(\bC\bA\bB),
$$
if all $\bA\bB\bC$, $\bB\bC\bA$, and $\bC\bA\bB$ exist.
Then, it follows that 
$
\Exp[\rvb^\top\bA\rvb]=\Exp[\trace(\bA\rvb\rvb^\top)]=\trace(\Exp[\bA\rvb\rvb^\top])= \trace(\bA\Exp[\rvb\rvb^\top]), 
$
where the second equality follows from the linear property of trace.
%\end{mdframed}

Let $\rvb\triangleq( \bX^\top\bX)^{-1}\bX^\top\bepsilon $ and $\bA\triangleq\Cov_{\rvx_\ast}[\rvx_\ast]$, then the second part of Equation~\eqref{equation:learning-curve-part3} is 
\begin{equation}
\begin{aligned}
&\Exp_{\bepsilon}\left[\trace\left( \bepsilon^\top \bX (\bX^\top\bX)^{-1}  \Cov_{\rvx_\ast}[\rvx_\ast] ( \bX^\top\bX)^{-1}\bX^\top\bepsilon\right)\right] 
=\Exp[\trace(\rvb^\top\bA\rvb)] 
= \trace(\bA\Exp[\rvb\rvb^\top]) \\
&\stackrel{\dag}{=} \trace\left(\Cov_{\rvx_\ast}[\rvx_\ast] \sigma^2 (\bX^\top\bX)^{-1}\right)  
=\frac{\sigma^2}{n} \trace\left(\Cov_{\rvx_\ast}[\rvx_\ast] \left(\frac{1}{n}\bX^\top\bX\right)^{-1}\right)
\stackrel{n\rightarrow \infty}{\longrightarrow} \frac{p}{n} \sigma^2,   \nonumber
\end{aligned}
\end{equation}
where the  equality $(\dag)$ follows from the assumption that $\bepsilon \sim \normal(\bzero, \sigma^2\bI)$, we have $\rvb\sim \normal(\bzero, \sigma^2 (\bX^\top\bX)^{-1})$ and $\Exp[\rvb\rvb^\top] = \Cov[\rvb] = \sigma^2 (\bX^\top\bX)^{-1}$. And the last equation follows from the fact that \textcolor{mydarkblue}{$(\frac{1}{n}\bX^\top\bX)$ converges to $\Cov_{\rvx_\ast}[\rvx_\ast]$} as $n \rightarrow \infty$, and the trace of a $p\times p$ identity matrix is $p$.

\paragraph{Step 3.2.} Similarly, for the third part of Equation~\eqref{equation:learning-curve-part3}, we have
\begin{equation}
\begin{aligned}
\Exp_{\bepsilon}\left[\Exp_{\rvx_\ast}[\rvx_\ast]^\top( \bX^\top\bX)^{-1}\bX^\top\bepsilon \bepsilon^\top \bX (\bX^\top\bX)^{-1}\Exp_{\rvx_\ast}[\rvx_\ast]\right] &= \frac{\sigma^2 \Exp[\rvx_\ast]^\top \Cov[\rvx_\ast] \Exp[\rvx_\ast] ]}{n}, \nonumber
\end{aligned}
\end{equation}
which is an order of $\mathcalO(\frac{1}{n})$. 

Finally, we reduce Equation~\eqref{equation:learning-curve-part3} to 
$
\frac{n+p}{n}\sigma^2 + \mathcalO(\frac{1}{n}). 
$
%TODO:
%\begin{equation}
%	\begin{aligned}
%		\Exp[\re_\ast^2] &=\Exp[\epsilon_\ast^2]-\Exp[2\epsilon_\ast] \rvx_\ast^\top(\bX^\top\bX)^{-1}\bX^\top\bepsilon + \Exp[( \bepsilon^\top\bX(\bX^\top\bX)^{-1}\rvx_\ast \rvx_\ast^\top(\bX^\top\bX)^{-1}\bX^\top\bepsilon)] \\
%		&=\sigma^2 + tr(\bX(\bX^\top\bX)^{-1}\rvx_\ast \rvx_\ast^\top(\bX^\top\bX)^{-1}\bX^\top \cdot \sigma^2\bI) \\
%		&=\sigma^2 + \sigma^2 tr(\bX(\bX^\top\bX)^{-1}\rvx_\ast \rvx_\ast^\top(\bX^\top\bX)^{-1}\bX^\top ) \\
%		&=\sigma^2 + \sigma^2 tr(\rvx_\ast \rvx_\ast^\top(\bX^\top\bX)^{-1}\bX^\top\bX(\bX^\top\bX)^{-1} ) \\
%		&=\sigma^2 + \sigma^2 tr(\rvx_\ast \rvx_\ast^\top(\bX^\top\bX)^{-1} ) \\
%		&=\sigma^2 + \frac{\sigma^2}{n} tr(\rvx_\ast \rvx_\ast^\top(\frac{1}{n}\bX^\top\bX)^{-1} ),  \nonumber
%	\end{aligned}
%\end{equation}
%where $(\frac{1}{n}\bX^\top\bX)^{-1}$ converges to 
This completes the proof.
\end{proof}

\begin{figure}[htp]
\centering
\includegraphics[width=0.65\textwidth]{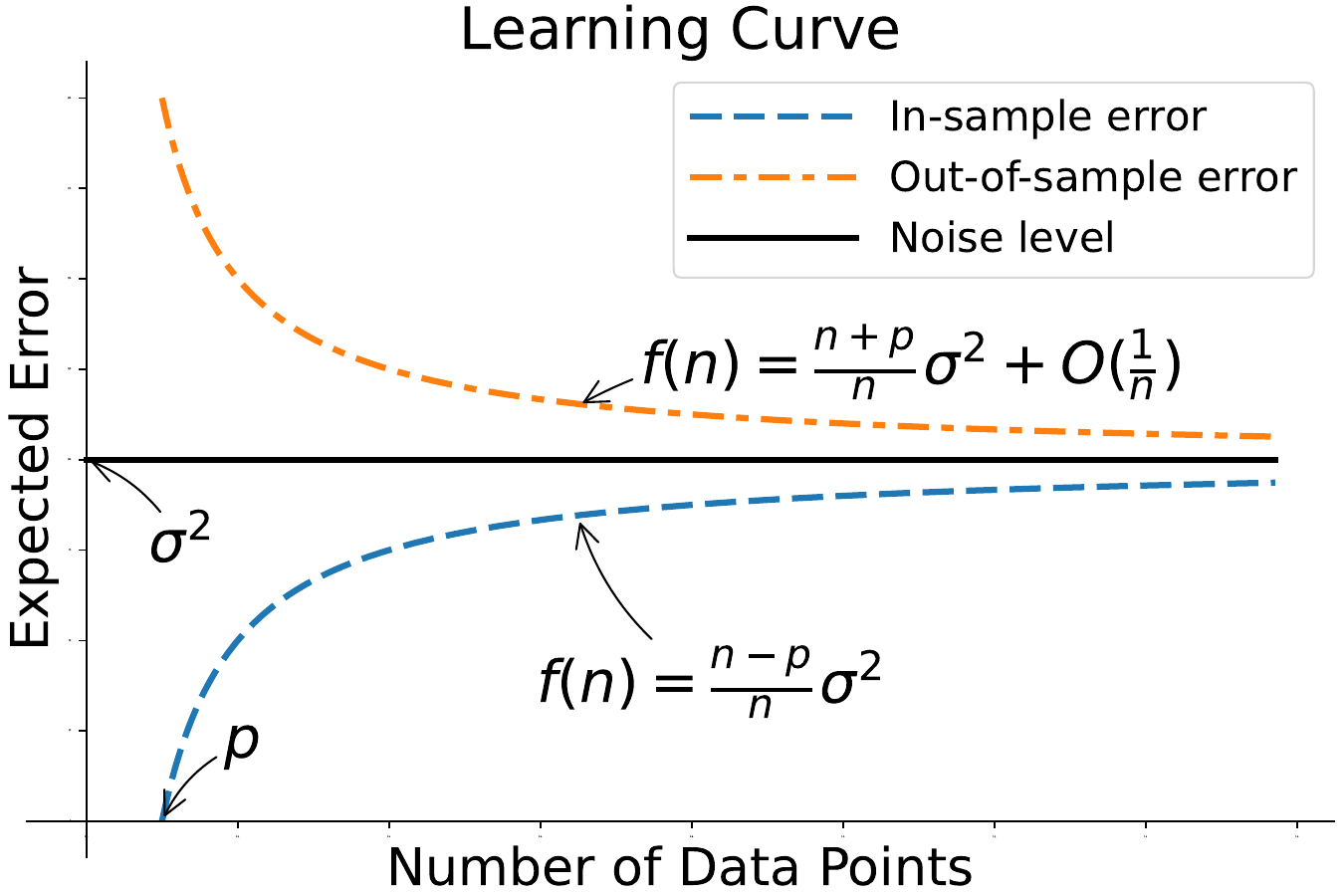}
\caption{Learning curve of least squares under Gaussian noise disturbance.}
\label{fig:learningcurve}
\end{figure}

Thus, we obtain the learning curve of least squares under Gaussian noise, as illustrated  in Figure~\ref{fig:learningcurve}. When the number of samples  $n$ significantly exceeds the dimension $p$, 
both the expected in-sample error and the expected out-of-sample error converge towards the noise level.

%\newpage
%\chapter{Large-Sample Properties in OLS Estimator*}\label{chapter:large_samp_ols}
%\begingroup
%\hypersetup{
%linkcolor=structurecolor,s
%linktoc=page,  % page: only the page will be colored; section, all, none etc
%}
%\minitoc \newpage
%\endgroup

\section{Large-Sample Properties in LS Estimator*}

\index{Large-sample properties}
\index{Convergence analysis}
\subsection{Convergence Results}\label{appendix:convergence-results}

We briefly introduce fundamental convergence results and asymptotic theory in this section. 
For more comprehensive discussions and proofs, see \citet{cameron2005microeconometrics, panaretos2016statistics, shao2006mathematical, Hansen2007, yu2021econometric, gut2009convergence, wooldridge2010econometric, hayashi2011econometrics, bickel2015mathematical}.

\subsection*{Convergence}
We begin by defining convergence in distribution:
\begin{definition}[Convergence in distribution]\index{Convergence in distribution}
Let $\{F_n\}_{n\geq 1}$ be a sequence of distribution functions and $G$ be a distribution function on $\real$. That is, the distribution function are
$$
F_n (x) = \prob[\rx_{n} \leq x],\ \forall \, n,
\qquad \text{and} \qquad 
G (x) = \prob[\rx \leq x].
$$ 
We say that $F_n$ converges in distribution to $G$, and write $F_n\stackrel{d}{\longrightarrow} G$ (or $\rx_n\stackrel{d}{\longrightarrow} \rx$) if and only if 
$$
F_n(x) \stackrel{n\rightarrow \infty}{\longrightarrow} G(x),
$$
for all $x$ that are continuity points of $G$.
\end{definition}

\begin{definition}[Convergence in probability]\label{definition:convg_prob}\index{Convergence in probability}
A sequence of random variables $\{\rx_n\}$ is said to converge in probability to random variable $\ry$ as $n \rightarrow \infty$, denoted  $\rx_n \stackrel{p}{\longrightarrow} \ry$, if for any $\epsilon >0$, 
$$
\prob[\abs{\rx_n-\ry} > \epsilon] \stackrel{n\rightarrow \infty}{\longrightarrow}0.
$$
\end{definition}
From the definitions above, it is clear that convergence in probability is a stronger condition than convergence in distribution. There are several other types of convergence in probability theory, such as convergence in $r$-mean and almost sure convergence. However, these will not be discussed further here.

\begin{example}
Let $\{\rx_n\}$ be a sequence of random variables with 
$$
\rx_n = (-1)^n \rx, \qquad \prob[\rx=-1]=\prob[\rx=1]=\frac{1}{2}.
$$
Then $\rx_n \stackrel{d}{\rightarrow}\rx$, but $\rx_n \stackrel{p}{\nrightarrow}\rx$.
\end{example}

\begin{lemma}[Convergence results]\label{lemma:equivalence-of-distribution-and-prob-convergence}
Let $\{\rx_n\}$ be a sequence of random variables. Then it follows that
\begin{enumerate}[(i)]
\item  $\rx_n \stackrel{p}{\rightarrow} \rx \,\,\Longrightarrow\,\, \rx_n \stackrel{d}{\rightarrow}\rx$
\item  $	\rx_n \stackrel{d}{\rightarrow}c \,\,\Longrightarrow\,\, \rx_n \stackrel{p}{\rightarrow}c, \qquad c\in \real.$
\end{enumerate}
\end{lemma}
\begin{proof}[of Lemma~\ref{lemma:equivalence-of-distribution-and-prob-convergence}]
\textbf{(i).} 
let $x$ be any continuity point of $F_\rx$, and let $\epsilon >0$. We begin by writing
\begin{equation}\label{equation:dist-prob-equa1}
\begin{aligned}
\prob[\rx_n \leq x] &= \prob[\rx_n \leq x, \textcolor{mylightbluetext}{\abs{\rx_n - \rx}\leq \epsilon}] + \prob[\rx_n \leq x, \abs{\rx_n-\rx} > \epsilon] \\
&= \prob[\rx_n \leq x,\textcolor{mylightbluetext}{\rx_n-\epsilon \leq  \rx\leq \rx_n+\epsilon}] + \prob[\rx_n \leq x, \abs{\rx_n-\rx} > \epsilon] \\
&\leq \prob[\rx\leq x+\epsilon] + \prob[\abs{\rx_n-\rx}> \epsilon],
\end{aligned}
\end{equation}
where the inequality comes from the fact that $\{\rx\leq x+\epsilon\}$ contains $\{\rx_n \leq x,\rx_n-\epsilon \leq  \rx\leq \rx_n+\epsilon\}$. Moreover, it follows that 
$$
\begin{aligned}
\prob[\rx \leq x-\epsilon] &= \prob[\rx\leq x-\epsilon, \textcolor{mylightbluetext}{\abs{\rx_n-\rx}\leq \epsilon}] + \prob[\rx\leq x-\epsilon, \abs{\rx_n-\rx}> \epsilon] \\
&=\prob[\rx\leq x-\epsilon,\textcolor{mylightbluetext}{ \rx-\epsilon \leq  \rx_n\leq \rx+\epsilon}] + \prob[\rx\leq x-\epsilon, \abs{\rx_n-\rx}> \epsilon] \\
&\leq \prob[\rx_n\leq x] + \prob[ \abs{\rx_n-\rx}> \epsilon],
\end{aligned} 
$$
where the inequality comes from the fact that $\{\rx_n\leq x\}$ contains $\{\rx\leq x-\epsilon, \rx-\epsilon \leq  \rx_n\leq \rx+\epsilon\}$. This implies
\begin{equation}\label{equation:dist-prob-equa2}
\prob[\rx \leq x-\epsilon] - \prob[ \abs{\rx_n-\rx}> \epsilon] \leq \prob[\rx_n\leq x].
\end{equation}
Combining \eqref{equation:dist-prob-equa1} and \eqref{equation:dist-prob-equa2} yields that 
$$
\prob[\rx \leq x-\epsilon] - \prob[ \abs{\rx_n-\rx}> \epsilon] \leq \prob[\rx_n\leq x] \leq  \prob[\rx\leq x+\epsilon] + \prob[\abs{\rx_n-\rx}> \epsilon].
$$
Since we assume $\prob[\abs{\rx_n-\rx} > \epsilon] \stackrel{n\rightarrow \infty}{\longrightarrow}0$. Then, when $n\rightarrow \infty$, we have 
$$
\prob[\rx \leq x-\epsilon]  \leq \prob[\rx_n\leq x] \leq  \prob[\rx\leq x+\epsilon],
$$
which yields (i).

\paragraph{(ii).} We also have 
$$
\begin{aligned}
\prob[\abs{\rx_n - c}> \epsilon] 
&= \prob[\rx_n - c > \epsilon] + \prob[\rx_n - c < -\epsilon]\\
&= \prob[\rx_n  > c+\epsilon] + \prob[\rx_n< c-\epsilon]\\
%&= 1-\prob[\rx_n  \leq c+\epsilon] + \prob[\rx_n< c-\epsilon]\\
&\leq  1-\prob[\rx_n  \leq c+\epsilon] + \prob[\rx_n\leq c-\epsilon]\\
&\stackrel{n\rightarrow \infty}{\longrightarrow} 1-F(c+\epsilon)+F(c-\epsilon).
\end{aligned}
$$
Since $\rx_n \stackrel{d}{\rightarrow}c$, we have $F(c+\epsilon)=1$ and $F(c-\epsilon)=0$ as $c>c+\epsilon$ and $c-\epsilon<c$. Therefore, 
$$
\prob[\abs{\rx_n - c}> \epsilon]\stackrel{n\rightarrow \infty}{\longrightarrow} 0,
$$
which completes the proof.
\end{proof}

\index{Joint convergence}
\begin{definition}[Joint convergence]
Let $\{\rvx_n\}$ be a sequence of random vectors of $\real^p$, and $\rvx$ be a random vector of $\real^p$. Define their distribution functions as 
$$
\begin{aligned}
F_{\rvx_n} (\bx) &= \prob[\rvx_{n1} \leq \bx_1, \rvx_{n2} \leq \bx_2, \ldots, \rvx_{np} \leq \bx_p];\\
F_{\rvx} (\bx)   &= \prob[\rvx_1 \leq \bx_1, \rvx_2 \leq \bx_2, \ldots, \rvx_p \leq \bx_p],
\end{aligned}
$$ 
for $\bx =[\bx_1, \bx_2, \ldots, \bx_p]^\top \in \real^p $. Then we say that $\rvx_n$ converges in distribution to $\rvx$ as $n\rightarrow \infty$ if for every continuity point of $F_{\rvx}$ we have
$$
F_{\rvx_n} \stackrel{n\rightarrow \infty}{\longrightarrow} F_{\rvx} (\bx).
$$
This convergence is denoted by $\rvx_n \stackrel{d}{\rightarrow}\rvx$.
\end{definition}

When an estimator approaches the true value as the sample size increases indefinitely, we refer to this property as consistency. In mathematical terms, this concept is defined as follows:
\index{Consistency}
\begin{definition}[Consistency]\label{definition:consist}
An estimator $\widehattheta_n$ of $\theta$, constructed on the basis of a sample of size $n$, is said to be \textit{consistent} if $\widehattheta_n  \stackrel{p}{\longrightarrow} \theta$ as $n\rightarrow \infty$.
\end{definition}

Consistency is a desirable attribute for an estimator. When the sample size is sufficiently large, the estimator will be very close to the true value with high probability. Additionally, the concentration of an estimator around the true parameter can always be bounded using the mean squared error.

\begin{lemma}\label{lemma:mse-inequality}
Let $\widehat{\btheta}$ be an estimator of $\btheta \in \real^p$ such that $\Var[\widehattheta_i] <\infty, \forall i \in \{1, 2, \ldots, p\}$. Then for all $\epsilon >0$, 
$$
\prob\left[\normtwobig{\widehat{\btheta} - \btheta} \geq \epsilon \right]\leq \frac{\MSE(\widehat{\btheta}, \btheta)}{\epsilon^2}.
$$
\end{lemma}
\begin{proof}[of Lemma~\ref{lemma:mse-inequality}]
Let $\rx \triangleq \normtwobig{\widehat{\btheta} - \btheta}^2$. Since $\epsilon >0$, applying Markov's inequality (see Section~\ref{section:inequalities}) yields
$$
\prob\left[\normtwobig{\widehat{\btheta} - \btheta}\geq\epsilon\right] = \prob[\rx \geq  \epsilon^2] \leq \frac{\Exp[\rx]}{\epsilon^2} = \frac{\Exp\left[\normtwobig{\widehat{\btheta} - \btheta}^2\right]}{\epsilon^2}=\frac{\MSE(\widehat{\btheta}, \btheta)}{\epsilon^2}.
$$
This completes the proof.
\end{proof}
\index{Biased estimator}
\index{Unbiased estimator}

Both biased and unbiased estimators can be consistent (an example can be found in Theorem~\ref{theorem:consistency-lse-noise-estimator}). 
Regarding mean squared error (MSE), both biased and unbiased estimators can achieve an MSE approaching zero as the sample size grows sufficiently large.

\begin{remark}[Consistency]
It is important to note that while the convergence of the MSE to zero implies consistency, the reverse implication does not generally hold.
\end{remark}

\index{Asymptotic theory}
\subsection*{Asymptotic Theory}
Five key tools, along with their extensions, play a central role in asymptotic theory.
These  include the weak law of large numbers (WLLN), the central limit theorem (CLT), the continuous mapping theorem (CMT), Slutsky's theorem, and the Delta method. 
We present the main results without proof; detailed proofs can be found, for example, in the references listed at the beginning of this section.

\index{Law of large numbers}
\begin{theoremHigh}[$L^2$ Weak Law of Large Numbers (WLLN2)]\label{theorem:l2weaklaw_large}
Let $\rx_1, \rx_2, \ldots, \rx_n$ be i.i.d. random variables with $\Exp[\rx_i]=\mu < \infty$ and $\Var[\rx_i]=\sigma^2 < \infty$. Let $\bar{\rx}_n = \frac{1}{n} \sum_{i=1}^{n}\rx_i$. Then, as $n\rightarrow \infty$,
$$
\bar{\rx}_n \stackrel{p}{\longrightarrow} \mu.
$$

\end{theoremHigh}
\begin{remark}[$L^1$ weak law of large numbers (WLLN1)]
Actually, the same conclusion can be drawn under weaker assumptions. Let $\rx_1, \rx_2, \ldots, \rx_n$ be i.i.d. random variables with  $\Exp[\rx_i]=\mu < \infty$. It suffices to assume that $\Exp[\abs{\rx_i}]< \infty$ rather than $\Var[\rx_i]< \infty$. 
Then, as $n\rightarrow \infty$,
$$
\bar{\rx}_n \stackrel{p}{\longrightarrow} \mu.
$$
This result extends naturally to the multivariate case. 
Let $\rvx_1, \rvx_2, \ldots, \rvx_n$ be i.i.d. random \textcolor{black}{vectors} with $\Exp[\rvx_i]=\bmu < \infty$ and $\Exp[\norm{\rvx_i}]< \infty$. Let $\bar{\rvx}_n = \frac{1}{n} \sum_{i=1}^{n}\rvx_i$. 
Then, as $n\rightarrow \infty$,
$$
\bar{\rvx}_n \stackrel{p}{\longrightarrow} \bmu.
$$
\end{remark}
The key observation in WLLN lies in the scaling by $\frac{1}{n}$, which reduces the variance of the sample mean to  $\sigma^2/n$. causing it to vanish as $n$ increases.
As a result, the sample mean converges in probability to the true mean $\mu$, or to the mean vector $\bmu$ in the multivariate setting.

\index{Central limit theorem}
\begin{theoremHigh}[Central limit theorem (CLT)]\label{theorem:clt_main}
Let $\rx_1, \rx_2, \ldots, \rx_n$ be i.i.d. random variables with $\Exp[\rx_i]=\mu < \infty$ and $\Var[\rx_i]=\sigma^2 < \infty$. Let $\bar{\rx}_n = \frac{1}{n} \sum_{i=1}^{n}\rx_i$. Then, as $n\rightarrow \infty$,
$$
\sqrt{n}(\bar{\rx}_n - \mu) \stackrel{d}{\longrightarrow} \normal(0, \sigma^2).
$$
Similarly, this result extends to the multivariate case. Let $\rvx_1, \rvx_2, \ldots, \rvx_n$ be i.i.d. random \textcolor{black}{vectors} with $\Exp[\rvx_i]=\bmu < \infty$ and $\Var[\rvx_i]=\bSigma$. Let $\bar{\rvx}_n = \frac{1}{n} \sum_{i=1}^{n}\rvx_i$. Then, as $n\rightarrow \infty$,
$$
\sqrt{n}(\bar{\rvx}_n - \bmu) \stackrel{d}{\longrightarrow} \normal(\bzero, \bSigma).
$$
\end{theoremHigh}
The CLT is stronger than the WLLN2 as $\sqrt{n}(\bar{\rx}_n - \mu) \stackrel{d}{\longrightarrow} \normal(0, \sigma^2)$ implies $\bar{\rx}_n \stackrel{p}{\longrightarrow} \mu$ (since $\bar{\rx}_n \stackrel{d}{\longrightarrow} \normal(\mu, \frac{1}{n}\sigma^2)$ in CLT). However, $\bar{\rx}_n \stackrel{p}{\longrightarrow} \mu$ does not provide any information about $\sqrt{n}\bar{\rx}_n$.

\begin{remark}\label{remark:asump_gau_chi}
Note that if a statistic $ \rs $ asymptotically follows a Gaussian distribution $ \rs \sim \normal(\Exp[\rs], \Var[\rs]) $, where $ \Exp[\rs] $ and $ \Var[\rs] $ are the expectation and variance of $ \rs $, respectively, then approximately we have:
$$
\frac{\rs - \Exp[\rs]}{\sqrt{\Var[\rs]}} \sim \normal(0, 1).
$$
According to the definition of the Chi-squared distribution (Definition~\ref{definition:chisquare_dist}), this implies
$$
\frac{(\rs - \Exp[\rs])^2}{\Var[\rs]} \sim \chi^2_{(1)}.
$$
If $ \rs $ is a vector $ \rvs = [\rs_1, \rs_2, \ldots, \rs_k]^\top \in\real^k$, the above conclusion can be written in vector form as:
$$
(\rvs - \Exp[\rvs])^\top \bV^{-1} (\rvs - \Exp[\rvs]) \sim \chi^2_{(k)},
$$
where $ \bV $ is the covariance matrix of $\rvs$ and must be nonsingular.
Note that $(\rvs - \Exp[\rvs])^\top \bV^{-1} (\rvs - \Exp[\rvs])$ is also called the square of \textit{Mahalanobis distance} between $\rvs$ and $\Exp[\rvs]$.
\end{remark}
\index{Mahalanobis distance}

\index{WCLT}
\index{Weighted sum central limit theorem}
In addition to the standard CLT, the following theorem presents a more general version of the central limit theorem that will often be useful.
\begin{theoremHigh}[Weighted sum central limit theorem (WCLT)]\label{theorem:weighted-sum-clt}
Let $\{\rx_n \}$ be a sequence of i.i.d. real random variables with $\Exp[\rx_i]=0$ and $\Var[\rx_i]=1$. Let $\{\eta_n\}$ be a sequence of real constants. If 
$$
\mathop{\sup}_{1\leq i\leq n} \frac{\eta_i^2}{\sum_{j=1}^{n}\eta_j^2} \stackrel{n\rightarrow \infty}{\longrightarrow}0,
$$
then
$$
\frac{1}{\sqrt{\sum_{i=1}^{n}\eta_i^2}} \sum_{i=1}^{n}\eta_i \rx_i  \stackrel{d}{\longrightarrow} \normal(0, 1).
$$
\end{theoremHigh}

\index{Continuous mapping theorem}
\index{CMT}
%\index{Mann-Wald theorem}
\begin{theoremHigh}[Continuous mapping theorem (CMT)]\label{theorem:cont_mapthep}
Given a sequence of random variables $\{\rx_n\}$. If $\rx$ is a random variable such that $\prob[\rx \in \mathcal{A}] = 1$ and $g: \real \rightarrow \real$ is continuous everywhere on $\mathcal{A}$, then,
$$
\begin{aligned}
	\rx_n &\stackrel{d}{\longrightarrow} \rx \,\,\Longrightarrow\,\, g(\rx_n) \stackrel{d}{\longrightarrow} g(\rx); \\
	\rx_n &\stackrel{p}{\longrightarrow} \rx \,\,\Longrightarrow\,\, g(\rx_n) \stackrel{p}{\longrightarrow} g(\rx).
\end{aligned}
$$
Similarly, the result can be extended to the multi-dimensional case. Given a sequence of random vectors $\{\rvx_n\}$. If $\rvx$ is a random vector such that $\prob[\rvx \in \mathcal{A}] = 1$ and $g: \real^p \rightarrow \real^k$ is continuous everywhere on $\mathcal{A}$, then, 
$$
\begin{aligned}
	\rvx_n &\stackrel{d}{\longrightarrow} \rvx \,\,\Longrightarrow\,\, g(\rvx_n) \stackrel{d}{\longrightarrow} g(\rvx); \\
	\rvx_n &\stackrel{p}{\longrightarrow} \rvx \,\,\Longrightarrow\,\, g(\rvx_n) \stackrel{p}{\longrightarrow} g(\rvx).
\end{aligned}
$$
\end{theoremHigh}
The CMT was proved by \citet{mann1943stochastic} and is sometimes referred to as the \textit{Mann-Wald Theorem}. Note that the CMT allows the function $g$ to be discontinuous but the probability of being at a discontinuity point is zero. For example, the function $g(u) = \frac{1}{u}$ is discontinuous at $u=0$. But if $\rx_n \stackrel{d}{\longrightarrow} \rx \sim \normal(0,1)$, then $\prob[\rx=0]=0$ such that $\rx_n^{-1} \stackrel{d}{\longrightarrow} \rx^{-1}$.  

\index{Slutsky's theorem}
\begin{theoremHigh}[Slutsky's theorem]
Let $\rx$ be a random variable and $c\in \real$ be a constant. Suppose $\rx_n \stackrel{d}{\rightarrow}\rx$ and $\ry_n \stackrel{d}{\rightarrow}c$ (i.e., $\ry_n \stackrel{p}{\rightarrow}c$). Then, it follows that 
$$
\begin{aligned}
	\rx_n+\ry_n &\stackrel{d}{\rightarrow} \rx+c;\\
	\rx_n\ry_n &\stackrel{d}{\rightarrow} c\rx; \\
	\ry_n^{-1}\rx_n &\stackrel{d}{\rightarrow} c^{-1}\rx,
\end{aligned}
$$
when $c\neq 0$, as $n\rightarrow \infty$. (Note here, $\rx_n$, $\ry_n$, $\rx$, and $c$ can be understood as vectors or matrices, provided that all operations are well-defined.)
\end{theoremHigh}

%\begin{remark}[Slutsky's theorem]
Note that one cannot replace the constant $c$ with a non-degenerate random variable (or random vector, matrix) in Slutsky's theorem. For example, take $\rx_n = -\rz + n^{-1}$ and $\ry_n = \rz-n^{-1}$, for $\rz\sim \normal(0,1)$. Then, $\rx_n\stackrel{d}{\rightarrow} \rz$ (since $-\rz \sim \normal(0,1)$), $\ry_n \stackrel{p}{\rightarrow} \rz$. But for all $n$, we have $\rx_n + \ry_n = 0$; and thus, $\rx_n + \ry_n$ fails to converge in distribution to $2\rz$.
%\end{remark}

An important application of Slutsky's theorem is given below:
\begin{example}
Suppose $\rvx_n \stackrel{d}{\rightarrow} \normal(\bzero, \bSigma)$ and $\rvy_n \stackrel{p}{\rightarrow}\bSigma$. Then $\rvy_n^{-1/2} \rvx_n  \stackrel{d}{\rightarrow} \bSigma^{-1/2} \normal(\bzero, \bSigma) = \normal(\bzero, \bI)$.
\end{example}

\begin{theoremHigh}[General version of Slutsky's theorem]
Let $\rx$ be a random variable and $c\in \real$ be a constant. Suppose $\rx_n \stackrel{d}{\rightarrow}\rx$ and $\ry_n \stackrel{d}{\rightarrow}c$ (i.e., $\ry_n \stackrel{p}{\rightarrow}c$). Let further that $g:\real \times \real$ be a continuous function. Then, it follows that 
$$
\begin{aligned}
	g(\rx_n,\ry_n) &\stackrel{d}{\rightarrow} g(\rx,c)
\end{aligned}
$$
as $n\rightarrow \infty$. (Note here, $\rx_n$, $\ry_n$, $\rx$, and $c$ can be understood as vectors or matrices as long as the operations are compatible.)
\end{theoremHigh}

\index{The Delta method}
The Delta method is a direct consequence of Slutsky's theorem and CMT. 
It allows us to apply transformations to results obtained from the central limit theorem.
\begin{theoremHigh}[The Delta method]
Let $\rz_n = a_n(\rx_n-\theta) \stackrel{d}{\rightarrow} \rz$, where $a_n, \theta \in \real$ for all $n$ and $a_n \uparrow \infty$. Let $g:\real \rightarrow \real$ be differentiable at $\theta$. Then, $a_n(g(\rx_n) - g(\theta)) \stackrel{d}{\rightarrow} g'(\theta)\rz$, provided that $g'(\theta) \neq 0$.
\end{theoremHigh}
In most applications, $a_n$ in the  theorem above is  constructed to be $\sqrt{n}$ such that $\sqrt{n} \uparrow \infty$ as $n\rightarrow \infty$. The Delta method implies that, asymptotically, the randomness in a transformation of $\rx_n$ is completely controlled by that in $\rx_n$.

\begin{theoremHigh}[The Delta method, multi-dimensional case]\label{theorem:multi-delta-method}
Let $\rvz_n = a_n(\rvx_n-\btheta) \stackrel{d}{\rightarrow} \rvz$, where $a_n\in \real$, $\btheta \in \real^d$ for all $n$ and $a_n \uparrow \infty$. Let $g:\real^d \rightarrow \real^p$ be continuously differentiable at $\btheta$. Then,
$$
a_n(g(\rvx_n) - g(\btheta)) \stackrel{d}{\rightarrow} J_g(\btheta)\rvz,
$$ 
where $J_g(\btheta)$ is the $p\times d$ Jacobian matrix of $g$, 
$$
J_g(\btheta) = 
\begin{bmatrix}
	\frac{\partial}{\partial x_1 }g_1(\btheta)& \frac{\partial}{\partial x_2 }g_1(\btheta)&  \ldots& \frac{\partial}{\partial x_d }g_1(\btheta) \\
	\frac{\partial}{\partial x_1 }g_2(\btheta)& \frac{\partial}{\partial x_2 }g_2(\btheta)&  \ldots& \frac{\partial}{\partial x_d }g_2(\btheta) \\
	\vdots & \vdots & \vdots &\vdots \\
	\frac{\partial}{\partial x_1 }g_p(\btheta)& \frac{\partial}{\partial x_2 }g_p(\btheta)&  \ldots& \frac{\partial}{\partial x_d }g_p(\btheta) \\
\end{bmatrix}.
$$
\end{theoremHigh}
\begin{proof}[of Theorem~\ref{theorem:multi-delta-method}]
By Taylor's expansion (Theorem~\ref{theorem:taylo_exp}) around $\btheta$, we have 
$$
g(\rvx_n) = g(\btheta) + J_g(\btheta_n^\ast) (\rvx_n-\btheta),
$$
where $\btheta_n^\ast$ lies between $\rvx_n$ and $\btheta$ such that 
$$
\abs{\btheta_n^\ast -\btheta} <\abs{\rvx_n-\btheta} = \abs{a_n^{-1}}\cdot\abs{a_n (\rvx_n-\btheta)} 
=  \abs{a_n^{-1}}\cdot \abs{\rvz_n} \Stackp \bzero,
$$
where the convergence follows from  Slutsky's theorem. Therefore, we have $\btheta_n^\ast \Stackp \btheta$. By applying CMT, we have $J_g(\btheta_n^\ast) \Stackp J_g(\btheta)$. Therefore, 
$$
\begin{aligned}
\rvz_n = a_n(g(\rvx_n)-g(\btheta)) &=  a_n J_g(\btheta_n^\ast) (\rvx_n-\btheta)\\
&= J_g(\btheta_n^\ast) a_n(\rvx_n-\btheta) \Stackd J_g(\btheta) \rvz,
\end{aligned}
$$
where the convergence follows from  Slutsky's theorem.
\end{proof}

\begin{example}
Suppose $\rx_1, \rx_2, \ldots, \rx_n$ are i.i.d. random variables with mean $\mu$ and variance $\sigma^2< \infty$. Let $\bar{\rx}_n = \frac{1}{n} \sum_{i=1}^{n}\rx_i$. By CLT, we have
$$
\sqrt{n}(\bar{\rx}_n - \mu) \stackrel{d}{\longrightarrow} \normal(0, \sigma^2).
$$
Given a continuously differentiable function $g(\cdot)$, the Delta method implies that 
$$
\sqrt{n}(g(\bar{\rx}_n) - g(\mu)) \stackrel{d}{\longrightarrow} \normal(0, \sigma^2 (g'(\mu))^2).
$$
Suppose further that $\{\ry_n\}$ is a sequence of random variables such that $\ry_n \Stackp \sigma$. Then, by Slutsky's theorem, we obtain
$$
\sqrt{n}\left(\frac{g(\bar{\rx}_n) - g(\mu)}{\ry_n}\right) \stackrel{d}{\longrightarrow} \normal(0,  (g'(\mu))^2).
$$
\end{example}

\index{Cram\'er-Wold device}
Apart from the five basic weapons, the following Cram\'er-Wold device is also very useful.
\begin{theoremHigh}[Cram\'er-Wold device]\label{theorem:cramer-rao-device}
Let $\{\rvx_n\}$ be a sequence of random vectors in $\real^p$, and $\rvx$ be a random vector in $\real^p$. Then,
$$
\rvx_n \stackrel{d}{\rightarrow}\rvx\qquad  \text{if and only if} \qquad \bu^\top \rvx_n \stackrel{d}{\rightarrow} \bu^\top \rvx,\ \forall \, \bu\in \real^p.
$$
\end{theoremHigh}

\subsection{Assumptions Restated}
For large sample asymptotic results of least squares estimators, we restate the assumptions as follows:
\begin{assumption}
We assume that 
\item OLS.0 (random sampling): ($\ry_i, \rvx_i$), $i=1,2,\ldots, n$ are i.i.d.
\item OLS.1 (full rank): $\rank(\bX) = p$.
\item \qquad  OLS.$1.2$ $\rank(\Exp[\rvx\rvx^\top] = p)$.
\item OLS.2 (first moment): $\Exp[\ry\mid \rvx] = \rvx^\top \bbeta$, i.e., $\ry = \rvx^\top \bbeta + \epsilon$ and $\Exp[\epsilon \mid \rvx] = 0$.
\item \qquad  OLS.$2.2$  $\ry = \rvx^\top \bbeta+\epsilon$ with $\Exp[\epsilon\rvx] = \bzero$.
\item OLS.3 (second moment): $\Exp[\epsilon_i^2] < \infty$, where $\epsilon_i$ is the noise associated with the input $\rvx_i$.
\item  \qquad OLS.$3.2$ (homoskedasticity): $\Exp[\epsilon_i^2 \mid  \rvx_i] = \sigma^2$.
\item OLS.4 (normality): $p(\epsilon \mid  \rvx) \sim \normal(0,\sigma^2)$.
\end{assumption}

Assumption OLS.2 is equivalent to stating $\ry = \rvx^\top \bbeta + \epsilon$ (linear in parameters) and $\Exp[\epsilon \mid \rvx]=0$ (zero conditional mean). 
And Assumption OLS.2 is stronger than Assumption OLS2.2.
Moreover, Assumption OLS.3.2 is stronger than Assumption OLS.3 since OLS3.2 implies the independence between $\epsilon_i$ and $\rvx_i$. The linear model under Assumption OLS.3.2 is called the \textit{homoskedastic linear regression model}~\footnote{
The term ``homoskedastic" is a concept from statistics and econometrics that refers to the condition where the variance of the error terms (or residuals) in a regression model is constant across all levels of the independent variables. The Greek roots of the word provide insight into its meaning: ``homos"  means ``same" and ``skedastic" means ``scatter" or ``spread."
So, when combined, ``homoskedastic" literally translates to ``same scatter" or ``same spread." This indicates that the variability or dispersion of the errors remains consistent throughout the range of the data.
}. 
In most of our discussions in the previous sections, we assume OLS.0 (with fixed $\bx_i$'s), OLS.1, OLS.2, OLS.3.2, OLS.4. That is
$$
\begin{aligned}
\ry &= \rvx^\top \bbeta + \epsilon ,\\
\Exp[\epsilon\mid  \rvx] &= 0, \\
\Exp[\epsilon^2 \mid  \rvx] &= \sigma^2, \\
p(\epsilon \mid  \rvx) &\sim \normal(0,\sigma^2).
\end{aligned}
$$
However, if $\Exp[\epsilon^2\mid  \rvx] = \sigma^2(\rvx)$, i.e., the noise variance depends on the input $\rvx$, then the errors are said to be  \textit{heteroskedastic}.

\index{Heteroskedastic}
\index{Homoskedastic}

\subsection{Asymptotics for the OLS  Estimator}\label{section:asymp_ols}

\index{Fisher information}
\subsection*{Asymptotic distribution of MLE}
To study the asymptotic behavior of the OLS estimator, we first examine the asymptotic properties of the MLE in general.
Let $\widehattheta$ be the MLE of $\theta$. 
To avoid confusion, let the true value of $\theta$ be denoted by $\theta_0$. 
We shall show that as the sample size $n$ becomes large, the distribution of the MLE $\widehattheta$ is approximately normal with mean $\theta_0$ and variance $1/n\fisher(\theta_0)$, where $\fisher(\theta_0)$ is the Fisher information evaluated at $\theta_0$. 
Since this result holds only in the limit as $n\rightarrow \infty$, we say that the MLE is \textit{asymptotically unbiased}, and we refer to the variance of the limiting normal distribution as the \textit{asymptotic variance} of the MLE. More precisely, we have the following theorem:

\begin{theoremHigh}[The asymptotic distribution of MLE]\label{theorem:asymp_of_mle}
Let $\rx_1, \rx_2, \ldots, \rx_n$ be a sample of size $n$ from a distribution for which the p.d.f. or p.m.f. is $p(x\mid \theta)$, where $\theta$ is an unknown parameter. 
Assume that the true value of $\theta$ is $\theta_0$, and the MLE of $\theta$ is $\widehattheta$. Then the probability distribution of $\sqrt{n \fisher(\theta_0)} (\widehattheta - \theta_0)$ converges to a standard normal distribution as $n\rightarrow \infty$. 
In other words, the asymptotic distribution of $\widehattheta$ is
$$
\normal \left( \theta_0, \frac{1}{n \fisher(\theta_0)} \right),
$$
where $\fisher(\cdot)$ denotes the Fisher information; see Section~\ref{section:mvu_esti}. 
\end{theoremHigh}
\begin{proof}[of Theorem~\ref{theorem:asymp_of_mle}]
We shall prove that
$
\sqrt{n \fisher(\theta_0)} (\widehattheta - \theta_0) \sim \normal(0, 1)
$
asymptotically. Here, we provide only a sketch of the proof; the full details are beyond the scope of this book.

Recall that  the log-likelihood function is
$
\ell(\theta) \triangleq \ell(\theta; \bx)= \sum_{i=1}^{n} \ln p(x_i\mid \theta)
$
and the MLE $\widehattheta$ satisfies $\ell'(\widehattheta) = 0$. We apply the linear approximation theorem of  $\ell'(\widehattheta)$ at the point $\theta_0$ (Theorem~\ref{theorem:linear_approx}), yielding
$
0 = \ell'(\widehattheta) \approx \ell'(\theta_0) + (\widehattheta - \theta_0) \ell''(\theta_0)
$.
Therefore, we have
$$
\widehattheta - \theta_0 \approx \frac{-\ell'(\theta_0)}{\ell''(\theta_0)}
\qquad \text{and}\qquad 
\sqrt{n} (\widehattheta - \theta_0) \approx \frac{-n^{-1/2} \ell'(\theta_0)}{n^{-1} \ell''(\theta_0)}.
$$
Now consider the numerator of this expression. Its expectation is
$$
\Exp[-n^{-1/2} \ell'(\theta_0)] 
= n^{-1/2} \sum_{i=1}^{n} \Exp\left[ \frac{\partial}{\partial \theta} \ln p(x_i\mid \theta_0) \right] 
= n^{-1/2} \sum_{i=1}^{n} \Exp\left[ \ell'(\theta_0; \rx_i) \right] = 0,
$$
and its variance is
$$
\Var[-n^{-1/2} \ell'(\theta_0)] 
= \frac{1}{n} \sum_{i=1}^{n} \Exp\left[ \Big(\frac{\partial}{\partial \theta} \ln p(x_i\mid \theta_0)\Big)^2  \right]
= \frac{1}{n} \sum_{i=1}^{n} \Exp\left[\big(\ell'(\theta_0; \rx_i)\big)^2 \right]
= \fisher(\theta_0).
$$
By definition, the denominator is $\frac{1}{n} \ell''(\theta_0) = \frac{1}{n} \sum_{i=1}^{n} \frac{\partial^2}{\partial \theta^2} \ln p(x_i\mid \theta_0)$.
By the law of large numbers (Theorem~\ref{theorem:l2weaklaw_large}), this expression converges to
$$
\Exp\left[ \frac{\partial^2}{\partial \theta^2} \ln p(\rx_i\mid \theta_0) \right] = -\fisher(\theta_0).
$$
Thus, we can write
$$
\sqrt{n} (\widehattheta - \theta_0) \approx \frac{-n^{-1/2} \ell'(\theta_0)}{\fisher(\theta_0)}.
$$
Therefore, it holds that 
$$
\begin{aligned}
\Exp\left[ \sqrt{n} (\widehattheta - \theta_0) \right]  &\approx \frac{\Exp[n^{-1/2} \ell'(\theta_0)]}{\fisher(\theta_0)} = 0; \\
\Var \left[ \sqrt{n} (\widehattheta - \theta_0) \right] &\approx \frac{\Var[n^{-1/2} \ell'(\theta_0)]}{\fisher^2(\theta_0)} = \frac{\fisher(\theta_0)}{\fisher^2(\theta_0)} = \frac{1}{\fisher(\theta_0)}.
\end{aligned}
$$
As $n \to \infty$, applying central limit theorem (Theorem~\ref{theorem:clt_main}), we have
$$
\sqrt{n} (\widehattheta - \theta_0) \sim \normal \left( 0, \frac{1}{\fisher(\theta_0)} \right)
\qquad \iff \qquad 
\sqrt{n \fisher(\theta_0)} (\widehattheta - \theta_0) \sim \normal(0, 1).
$$
This completes the proof.
\end{proof}

This result shows that the MLE is asymptotically optimal, because its asymptotic variance reaches the CRLB; see Theorem~\ref{theorem:crlb}. For this reason, MLE is frequently used especially with large samples.
For a multivariate problem, similar to the scalar case, the asymptotic distribution of the MLE $\widehat{\btheta}_{\text{MLE}}$ is approximately multivariate Gaussian distribution with the true value of $\btheta_0$ as the mean and $[n\fisher(\btheta_0)]^{-1}$ as the covariance matrix.
Once again,  $\widehatbbeta = (\bX^\top \bX)^{-1} \bX^\top \by$ stands the maximum likelihood estimate of the OLS problem.
This further illustrates why the method of ordinary least squares plays such a central role in linear model estimation.

\subsection*{Consistency of OLS}

\index{Consistency}
We now show that the OLS estimator is consistent under relatively mild conditions.
\begin{theoremHigh}[Consistency of OLS estimator]\label{theorem:consistency-lse}
Let $\{\rmX_m\}$ be a sequence of $n\times p$ design matrices, and $\{\bepsilon_m\}$ be a sequence of $n\times 1$ vectors. For each element, let $\rvy_m = \rmX_m\bbeta + \bepsilon_m$, and the OLS estimator is obtained by $\widehatbbeta_m = (\rmX_m^\top \rmX_m)^{-1} \rmX_m^\top \rvy_m$. Suppose the model satisfies the following assumptions:
\begin{enumerate}[(i)]
\item  OLS.0 (random sampling): ($\ry_i, \rvx_i$), $i=1,2,\ldots, n$ are i.i.d.
\item  OLS.1.2 (full rank): $\rank(\Exp[\rvx\rvx^\top] = p)$.
\item  OLS.2.2: $\ry = \rvx^\top \bbeta+\epsilon$ with $\Exp[\epsilon\rvx] = \bzero$. 
\item  OLS.3 (second moment): $\Exp[\epsilon_i^2] < \infty$, where $\epsilon_i$ is the noise associated with input $\rvx_i$.
\end{enumerate}
\item Then the least squares estimator satisfies
$\widehatbbeta_m \Stackp \bbeta$, i.e., $\widehatbbeta_m$ is consistent.
\end{theoremHigh}
The assumptions OLS2.2 and OLS.3 in this context are relatively mild compared to those used earlier when deriving the OLS estimator $\widehatbbeta$. 
A stronger set of  assumptions, such as OLS.2 and OLS.3.2, would also lead to the same consistency result.
\begin{proof}[of Theorem~\ref{theorem:consistency-lse}]
We begin by expressing the OLS estimator as:
$$
\begin{aligned}
\widehatbbeta_m &= (\rmX_m^\top \rmX_m)^{-1}\rmX_m^\top \rvy_m = (\rmX_m^\top \rmX_m)^{-1} \rmX_m^\top (\rmX_m\bbeta + \bepsilon_m)
= \bbeta + (\rmX_m^\top \rmX_m)^{-1} \rmX_m^\top\bepsilon_m.
\end{aligned}
$$
To prove consistency, it suffices to show that the second term converges in probability to zero: $(\rmX_m^\top \rmX_m)^{-1} \rmX_m^\top\bepsilon_m \Stackp \bzero $.

For $\rmX_m^\top \rmX_m $, since $\rvx_i$'s are i.i.d. (OLS.0),
$\rmX_m^\top \rmX_m = \frac{1}{n} \sum_{i=1}^{n} \rvx_i \rvx_i^\top \Stackp \Exp[\rvx_i \rvx_i^\top] $~\footnote{Note again we let $\rvx_i$ be the $i$-th row of $\rmX_m$.} as $\Exp[\norm{\rvx_i}^2] \leq \infty$ implied in OLS.1.2 and by applying WLLN1.  This implies $(\rmX_m^\top \rmX_m)^{-1} \Stackp \Exp[\rvx_i \rvx_i^\top]^{-1}$.

For $\rmX_m^\top \bepsilon_m$, similarly, since $\rvx_i$'s are i.i.d. (OLS.0), we have $\rmX_m^\top \bepsilon_m = \frac{1}{n} \sum_{i=1}^{n} \epsilon_i\rvx_i  \Stackp \Exp[\epsilon_i\rvx_i] $ by applying WLLN1 if we have $\Exp[\norm{\epsilon\rvx}] < \infty$. To see why  $\Exp[\norm{\epsilon\rvx}] < \infty$, we have 
$$
\Exp[\norm{\epsilon\rvx}] \leq \Exp\left[\norm{\rvx}^2\right]^{1/2} \Exp\left[|\epsilon|^2\right]^{1/2} <\infty,
$$
by Cauchy-Schwarz inequality (see Section~\ref{section:inequalities}), and which is finite by Assumption OLS1.2 and OLS.3. Then, $(\rmX_m^\top\rmX_m)^{-1} \rmX_m^\top \bepsilon_m$ can be expressed as
$$
\begin{aligned}
(\rmX_m^\top\rmX_m)^{-1} \rmX_m^\top \bepsilon_m &= \left( \sum_{i=1}^{n} \rvx_i \rvx_i^\top \right)^{-1} \left( \frac{1}{n} \sum_{i=1}^{n} \epsilon_i\rvx_i \right) 
\Stackp \Exp[\rvx_i \rvx_i^\top]^{-1} \bzero = \bzero,
\end{aligned}
$$
by applying Slutsky's theorem. This completes the proof.
\end{proof}

\index{Slutsky's theorem}
\index{Consistency}
\begin{theoremHigh}[Consistency of noise estimators]\label{theorem:consistency-lse-noise-estimator}
Let $\{\rmX_m\}$ be a sequence of $n\times p$ design matrices, and $\{\bepsilon_m\}$ be a sequence of $n\times 1$ vectors. For each element, let $\rvy_m = \rmX_m\bbeta + \bepsilon_m$. 
The maximum likelihood estimator of the noise variance is obtained by 
$$
\widehat{\sigma}^2_m=  \frac{1}{n} \rve_m^\top\rve_m = \frac{1}{n} (\rvy_m - \rmX_m\widehatbbeta_m)^\top (\rvy_m - \rmX_m\widehatbbeta_m), \qquad \text{(Theorem~\ref{theorem:mle-gaussian})}
$$
and an unbiased estimator of the noise variance is given by
$$
\rS^2_m = \frac{1}{n-p} \rve_m^\top\rve_m =  \frac{1}{n-p} (\rvy_m - \rmX_m\widehatbbeta_m)^\top (\rvy_m - \rmX_m\widehatbbeta_m). \qquad \text{(Theorem~\ref{theorem:ss-chisquare})}
$$ 
Now suppose that the linear regression model satisfies the same assumptions as in Theorem~\ref{theorem:consistency-lse}, along with the following additional condition:
\item \gap (v). OLS.$3.2$ (homoskedasticity): $\Exp[\epsilon_i^2 \mid  \rvx_i] = \sigma^2$.
\item Then, both estimators of the noise variance are consistent:
$\widehat{\sigma}^2_m \Stackp \sigma^2$ and $\rS^2_m  \Stackp \sigma^2$, which means $\widehat{\sigma}^2_m$ and $\rS^2_m$ are consistent estimators of $\sigma^2$.
\end{theoremHigh}
\begin{proof}[of Theorem~\ref{theorem:consistency-lse-noise-estimator}]
We begin by expressing the residual vector as
$$
\rve_m = \rvy_m - \rmX_m\widehatbbeta_m = (\bepsilon_m + \rmX_m \bbeta) - \rmX_m\widehatbbeta_m = \bepsilon_m + \rmX_m(\bbeta - \widehatbbeta_m).
$$
Thus, it follows that
$$
\begin{aligned}
\widehat{\sigma}^2_m 
&= \frac{1}{n}\rve_m^\top \rve_m 
= \frac{1}{n} \left\{\bepsilon_m^\top \bepsilon_m + 2\bepsilon_m^\top\rmX_m(\bbeta - \widehatbbeta_m) + (\bbeta - \widehatbbeta_m)^\top \rmX_m^\top \rmX_m(\bbeta - \widehatbbeta_m) \right\} \\
&= \frac{1}{n} \sum_{i=1}^{n} \epsilon_i^2 + 2 \left(\frac{1}{n}\sum_{i=1}^{n} \epsilon_i\rvx_i^\top\right)(\bbeta - \widehatbbeta_m) +  (\bbeta - \widehatbbeta_m)^\top  \left(\frac{1}{n} \sum_{i=1}^{n}\rvx_i\rvx_i^\top \right) (\bbeta - \widehatbbeta_m)  \\
&\Stackp \sigma^2,
\end{aligned}
$$
by applying the WLLN1, Theorem~\ref{theorem:consistency-lse}, and  Slutsky's theorem.

For $\rS_m^2$, it follows that 
$
\rS_m^2 = \frac{n}{n-p}\widehat{\sigma}^2_m \Stackp \sigma^2
$
by  Slutsky's theorem.
\end{proof}

From the  above theorem, we find two different estimators---one biased and one unbiased---can  both be consistent. That is, when sample size $n$ is sufficiently large, the estimates $\widehat{\sigma}^2_m$ and $\rS_m^2$ become very close to each other and to the true value $\sigma^2$.

As previously noted, unbiasedness is not a necessary condition for consistency. This result highlights that even though $\widehat{\sigma}^2_m$ is biased in finite samples, it still converges to the correct parameter value as $n$ increases. Similarly, although $\rS_m^2$ is unbiased, its consistency also depends on the behavior of the estimator $\widehatbbeta_m$ and the law of large numbers.

\index{Asymptotic normality}
\subsection*{Asymptotic Normality under Noise Moment Assumption}
To study the approximate sampling distribution of the OLS estimator under the moment assumption, we introduce the following additional assumption:
\begin{description}\centering
\item Assumption OLS.5: $\mathop{\max}_{1\leq i\leq n}\left\{ \rvx_i^\top (\rmX_m^\top\rmX_m)^{-1}\rvx_i \right\}\stackrel{n\rightarrow \infty}{\longrightarrow}0$.
\end{description}
This indicates that the diagonal elements of the hat matrix $\bH_m=\bX_m(\bX_m^\top\bX_m)^{-1}\bX_m^\top$ converge to zero.

Under the moment assumption, we only assume that the noise vector satisfies $\Exp[\bepsilon] = \bzero$ and $\Cov[\bepsilon]=\sigma^2\bI$, rather than assuming normality:  $\bepsilon \sim \normal(\bzero, \sigma^2 \bI)$.
We proved that the OLS estimator of $\bbeta$ is unbiased such that $\Exp[\widehatbbeta] = \bbeta$ and $\Exp[\rve]=\bzero$ in Section~\ref{section:unbiasedness-moment-assumption}, and the sampling distribution of $\widehatbbeta$ is $\widehatbbeta \sim \normal(\bbeta, \sigma^2(\bX^\top \bX)^{-1})$ in Theorem~\ref{theorem:samplding_dist_lse_gaussian}  if we assume further that the noise follows from a Gaussian distribution: $\bepsilon \sim \normal(\bzero, \sigma^2 \bI)$.
However, what can we say about the sampling distribution of $\widehatbbeta$ when the noise vector is not normally distributed?

We now show that, under certain regularity conditions (including Assumption OLS.5), we can approximate the sampling distribution of the OLS estimator even without assuming Gaussian errors.
\begin{theoremHigh}[Large sample distribution of $\widehatbbeta_m$]\label{theorem:larg-sample-hat-beta}
Let $\{\rmX_m\}$ be a sequence of $n\times p$ design matrices, and $\{\bepsilon_m\}$ be a sequence of $n\times 1$ vectors. For each element, let $\rvy_m = \rmX_m\bbeta + \bepsilon_m$. If
\begin{enumerate}[(i)]
\item  OLS.0 (random sampling): ($\ry_i, \rvx_i$), $i=1,2,\ldots, n$ are i.i.d.

\item  OLS.1.2: $\rmX_m$ has full rank $p$ for all $m\geq 1$.

\item  OLS.2 and OLS.3, the moment assumption: $\Exp[\bepsilon_m] = \bzero$ and $\Cov[\bepsilon_m]=\sigma^2\bI_m$ for all $m\geq 1$. That is, for each $i$-th element $\epsilon_{mi}$, we have $\Exp[\epsilon_{mi}] = 0$ and $\Var[\epsilon_{mi}]=\sigma^2$.

\item  OLS.5: $\mathop{\max}_{1\leq i\leq n}\left\{ \rvx_i^\top (\rmX_m^\top\rmX_m)^{-1}\rvx_i \right\}\stackrel{n\rightarrow \infty}{\longrightarrow}0$.
\end{enumerate}

\item Then the ordinary least squares estimator $\widehatbbeta_m = (\rmX_m^\top\rmX_m)^{-1}\rmX_m^\top\rvy_m$ satisfies 
$$
(\rmX_m^\top\rmX_m)^{1/2}(\widehatbbeta_m-\bbeta) \stackrel{d}{\longrightarrow} \normal(\bzero, \sigma^2\bI_p),
$$
where $\bI_p$ is a $p\times p$ identity matrix.
\end{theoremHigh}
\begin{proof}[of Theorem~\ref{theorem:larg-sample-hat-beta}]
Recall that $\widehatbbeta_m = (\rmX_m^\top\rmX_m)^{-1}\rmX_m^\top (\rmX_m \bbeta +\bepsilon_m)$. 
Then, we have $(\rmX_m^\top\rmX_m)^{1/2}(\widehatbbeta_m - \bbeta) = (\rmX_m^\top\rmX_m)^{-1/2}\rmX_m^\top \bepsilon_m$.
Let $\bu\in\real^p$ be a unit-length vector, and define the vector $\boldeta_m \in \real^n$ such that  
$$
\begin{aligned}
\boldeta_m &= [\eta_{m1}, \eta_{m2}, \ldots, \eta_{mn}]^\top \\
&= [\bu^\top(\rmX_m^\top\rmX_m)^{-1/2}\rvx_1, \ldots \bu^\top (\rmX_m^\top\rmX_m)^{-1/2}\rvx_n]^\top\\
&= \rmX_m (\rmX_m^\top\rmX_m)^{-1/2} \bu,
\end{aligned}
$$
where $\rvx_i$ is the $i$-th row of $\rmX_m$ for $i\in\{1,2,\ldots,n\}$.
Then, we have 
$$
\boldeta_m^\top\boldeta_m =  \bu^\top (\rmX_m^\top\rmX_m)^{-1/2} (\rmX_m^\top\rmX_m) (\rmX_m^\top\rmX_m)^{-1/2} \bu=1.
$$
That is, $\boldeta_m$ is  a unit-length random vector.
Moreover, according to Schwarz matrix inequality (see Section~\ref{section:inequalities}), we have
$$
\eta_{mi}^2 \leq \normtwo{\bu}^2 \normtwo{(\rmX_m^\top\rmX_m)^{-1/2}\rvx_i}^2 = \rvx_i^\top (\rmX_m^\top\rmX_m)^{-1}\rvx_i.
$$
Therefore, by assumption:
$$
\mathop{\max}_{1\leq i\leq n}\frac{ \eta_{mi}^2}{\sum_{j=1}^{j}\eta_{mj}^2} =\mathop{\max}_{1\leq i\leq n}\frac{ \eta_{mi}^2}{\boldeta_m^\top\boldeta_m} \leq \mathop{\max}_{1\leq i\leq n}\rvx_i^\top  (\rmX_m^\top\rmX_m)^{-1}\rvx_i \stackrel{n\rightarrow \infty}{\longrightarrow}0.
$$
Thus, by weighted sum central limit theorem in Theorem~\ref{theorem:weighted-sum-clt}, it follows that
$$
\sum_{i=1}^{n}\eta_{mi} \epsilon_{mi} = \boldeta_m^\top \bepsilon_m \stackrel{d}{\longrightarrow} \normal(0, \sigma^2).
$$
That is,
$$
\bu^\top (\rmX_m^\top\rmX_m)^{-1/2} \rmX_m^\top\bepsilon_m \stackrel{d}{\longrightarrow} \normal(0, \sigma^2)
$$
Since we assume $\bu$ is a unit-length vector, then if $\rvy\sim \normal(\bzero, \sigma^2\bI)$, we have $\bu^\top\rvy\sim \normal(0, \sigma^2)$. By Cram\'er-Wold device (see Theorem~\ref{theorem:cramer-rao-device}), this implies that 
$$
(\rmX_m^\top\rmX_m)^{-1/2} \rmX_m^\top\bepsilon_m \stackrel{d}{\longrightarrow} \normal(\bzero, \sigma^2\bI).
$$
We complete the proof.
\end{proof}

The result above implies that, for large enough sample size $n$, the OLS estimator approximately follows a multivariate normal distribution:
$$
\widehatbbeta_m \stackrel{d}{\longrightarrow} \normal(\bbeta, \sigma^2(\bX_m^\top\bX_m)^{-1}).
$$
Similarly, the predicted values satisfy:
$$
\widehat{\rvy}_m = \bX_m\widehatbbeta_m \stackrel{d}{\longrightarrow} \normal(\bX_m\bbeta, \sigma^2\bX_m(\bX_m^\top\bX_m)^{-1}\bX_m^\top) = \normal(\bX_m\bbeta, \sigma^2\bH_m).
$$
This asymptotic normality justifies the use of standard inference procedures (e.g., hypothesis testing, confidence intervals) even when the errors are not normally distributed, provided the regularity conditions hold.

\index{Asymptotic normality}
\subsection*{Asymptotic Normality under Higher Moment Assumption}
To further investigate the approximate sampling distribution of the OLS estimator under the moment assumption, we introduce the following additional condition:
\begin{description}
\centering
\item Assumption OLS.6: $\quad \Exp[\epsilon^4] < \infty$ and $\Exp[\norm{\rvx}^4] <\infty$.
\end{description}

\begin{theoremHigh}[Large sample distribution of $\widehatbbeta_m$]\label{lemma:larg-sample-hat-beta2}
Let $\{\rmX_m\}$ be a sequence of $n\times p$ design matrices, and $\{\bepsilon_m\}$ be a sequence of $n\times 1$ vectors. For each element, let $\rvy_m = \rmX_m\bbeta + \bepsilon_m$. 
Suppose the model satisfies the following assumptions:
\begin{enumerate}[(i)]
\item  OLS.0 (random sampling): ($\ry_i, \rvx_i$), $i=1,2,\ldots, n$ are i.i.d.
\item  OLS.1.2 (full rank): $\rank(\Exp[\rvx\rvx^\top] = p)$.
\item  OLS.2.2: $\ry = \rvx^\top \bbeta+\epsilon$ with $\Exp[\epsilon\rvx] = \bzero$.
\item  OLS.3 (second moment): $\Exp[\epsilon_i^2] < \infty$, where $\epsilon_i$ is the noise associated with input $\rvx_i$.
\item  OLS.6: $\Exp[\epsilon^4] < \infty$ and $\Exp[\norm{\rvx}^4] <\infty$.
\end{enumerate}
\item Then, the ordinary least squares estimator $\widehatbbeta_m = (\rmX_m^\top\rmX_m)^{-1}\rmX_m^\top\rvy_m$ satisfies 
$$
\sqrt{n} (\widehatbbeta_m-\bbeta) \Stackd \normal(\bzero, \bV),
$$
where $\bV = \bA^{-1} \bOmega \bA^{-1}$ with $\bA = \Exp[\rvx\rvx^\top] $ and $\bOmega = \Exp[\epsilon^2 \rvx\rvx^\top ]$.
\end{theoremHigh}
\begin{proof}[of Theorem~\ref{lemma:larg-sample-hat-beta2}]
From Theorem~\ref{theorem:consistency-lse}, we know that:
$$
\sqrt{n} (\widehatbbeta_m-\bbeta) = \sqrt{n}(\rmX_m^\top\rmX_m)^{-1} \rmX_m^\top \bepsilon_m = \left( \frac{1}{n} \sum_{i=1}^{n} \rvx_i \rvx_i^\top \right)^{-1} \left( \frac{1}{\sqrt{n}}  \sum_{i=1}^{n} \epsilon_i\rvx_i \right).
$$	
By applying Cauchy-Schwarz inequality, we have 
$$
\Exp\left[\norm{\rvx_i\rvx_i^\top \epsilon _i^2} \right] \leq \Exp\left[\norm{\rvx_i\rvx_i^\top }^2\right]^{1/2} \Exp\left[\epsilon_i^4\right]^{1/2} \leq \Exp\left[\norm{\rvx_i}^4\right]^{1/2} \Exp\left[\epsilon_i^4\right]^{1/2} < \infty,
$$
where the second inequality follows from the Schwarz matrix inequality (see Section~\ref{section:inequalities}), and the last inequality holds due to Assumption OLS.6.
Therefore, according to the CLT and OLS.2.2, we have
$$
\frac{1}{\sqrt{n}}  \sum_{i=1}^{n} \epsilon_i\rvx_i  = \sqrt{n} \left( \frac{1}{n}\sum_{i=1}^{n} \epsilon_i\rvx_i - \bzero \right)  \Stackd \normal(\bzero, \bOmega).
$$
Again, we have $ \frac{1}{n} \sum_{i=1}^{n} \rvx_i \rvx_i^\top \Stackp \bA$. Thus,
$$
\sqrt{n} (\widehatbbeta_m-\bbeta)  \Stackd \bA^{-1} \normal(\bzero, \bOmega) = \normal(\bzero, \bV)
$$
by Slutsky's theorem and the symmetry of $\bA$.
\end{proof}

Compared to Theorem~\ref{theorem:larg-sample-hat-beta} (which assumes \textit{homoskedasticity}), we have a similar result in the  Theorem~\ref{lemma:larg-sample-hat-beta2} above (which assumes \textit{heteroskedasticity}). However, in Theorem~\ref{lemma:larg-sample-hat-beta2}, we do not assume the moment assumption on the noise: $\Exp[\bepsilon_m] = \bzero$ and $\Cov[\bepsilon_m]=\sigma^2\bI_n$. This results in the difference between the covariance matrices of $\widehatbbeta_m$ in the two theorems.
\index{Homoskedastic}
\index{Heteroskedastic}

In the homoskedastic model, $\bV$ reduces to $\bV_0 = \sigma^2 \bA^{-1}$, which is the same as that in Theorem~\ref{theorem:larg-sample-hat-beta}. \footnote{$\bV_0 = \sigma^2 \bA^{-1}$ is known as the \textit{homoskedastic covariance matrix}.} But the result in Theorem~\ref{lemma:larg-sample-hat-beta2} is more general, as it allows for arbitrary patterns of heteroskedasticity, provided the fourth-moment conditions hold.

\begin{problemset}

\item \textbf{Gaussian MLE.} Let $x_1, x_2, \ldots, x_n$ be i.i.d. samples drawn from a Gaussian density $\normal(\mu, \sigma^2)$ (Definition~\ref{definition:gaussian_distribution}).
The unique MLE of $[\mu, \sigma^2]$ is given by 
$$
[\widehat{\mu}, \widehat{\sigma}^2] = \left[\bar{x},  \frac{1}{n}(x_i - \bar{x})^2\right],
$$
where $\bar{x} \triangleq \frac{1}{n}\sum_{i=1}^{n}x_i$. Show that the negative Hessian is positive definite.

\item \label{problem:exponential_mle} \textbf{Exponential MLE.} Let $x_1, x_2, \ldots, x_n$ be i.i.d. samples drawn from an exponential density $\exponential(\lambda)$ (Definition~\ref{definition:exponential_distribution}).
Show that the unique MLE of $\lambda$ is given by 
$$
\widehat{\lambda} = \left(\frac{1}{n}\sum_{i=1}^{n} x_i\right)^{-1} = \frac{1}{\bar{x}}.
$$

\item \textbf{Uniform noise linear model.} Suppose noise terms are i.i.d. with $\epsilon_i\sim \text{Uniform}(-a,a)$ for $i\in\{1,2,\ldots,n\}$, i.e., uniformly distributed on the interval $[-a,a]$ and the density function is $p(x)=\frac{1}{2a}$ for $x\in[-a,a]$. Show that a maximum likelihood estimate of the linear model is any $\bbeta$ satisfying $\norm{\bX\bbeta-\by}_\infty \leq a$.

\item Complete the proof of Theorem~\ref{theorem:mle-gaussian} by verifying that all second partial derivatives of the log-likelihood function are negative, confirming that the critical point corresponds to a maximum.

%\item Prove the Cram\'er-Rao lower bound theorem.

%% below are distribution theorem problems
\item Show that $\rx\sim \chi_{(2)}^2$ if and only if $\rx\sim \exponential(1/2)$ (Definition~\ref{definition:exponential_distribution}).

\item Prove that if $\rx\sim t_{(n)}$, then $\rx^2\sim F_{1,n}$.

\item In Theorem~\ref{theorem:mle-gaussian}, we derived the maximum likelihood estimator of $\sigma^2$ as: $\widehat{\sigma}^2=\frac{1}{n}(\rvy-\bX\widehatbbeta)^\top(\rvy-\bX\widehatbbeta)=\frac{1}{n}\rve^\top\rve = \frac{1}{n}\sum_i^n \re_i^2$. By Theorem~\ref{theorem:ss-chisquare}, we show that
%we have $\widehat{\sigma}^2 \sim \frac{\sigma^2}{n} \chi^2_{(n-p)}$ and $\Exp[\widehat{\sigma}^2] = \frac{n-p}{n} \sigma^2$.
the maximum likelihood estimator of $\sigma^2$ is a biased estimator. The unbiased estimator of $\sigma^2$ should be 
$\rS^2=\frac{1}{n-p} (\rvy-\bX\widehatbbeta)^\top(\rvy-\bX\widehatbbeta) = \frac{1}{n-p}\rve^\top\rve$.
Using the definitions of overestimation and underestimation provided in Definition~\ref{defintion:biased_unbiased}, determine whether the MLE $\widehat{\sigma}^2$ overestimates or underestimates the true variance $\sigma^2$.

%%%%%%%%%%%%% Below are large sample problems %%%%%%%%%%%%%%%%%%%%%%%%%%%%%

\item Prove Remark~\ref{remark:asump_gau_chi} rigorously.

\item Suppose $\rvx_n \Stackd \normal(\bzero, \bSigma)$ and $\rmY_n \Stackp \bSigma$, where $\rvx_n\in\real^k$ and $\rmY_n\in\real^{k\times k}$. Show that 
$\rvx_n^\top\rmY_n^{-1}\rvx_n \Stackd \chi_{(k)}^2$.

\item 
Using the classical central limit theorem (CLT, Theorem~\ref{theorem:clt_main}), prove the validity of the weighted sum CLT (Theorem~\ref{theorem:weighted-sum-clt}), which states that under suitable conditions, a weighted sum of independent random variables converges in distribution to a normal distribution.

\item 
Let $\rx_1, \rx_2, \ldots, \rx_n$ be i.i.d. random variables defined on the interval $[0, 1]$,  with probability density function:
$$
p(x\mid \theta) = \frac{\Gamma(2\theta)}{\Gamma(\theta)^2} [x(1-x)]^{\theta-1},
$$
where $\theta > 0$ is a parameter to be estimated from the sample. It can be shown that
$$
\Exp[\rx] = \frac{1}{2}
\qquad\text{and}\qquad 
\Var[\rx] = \frac{1}{4(2\theta + 1)}.
$$
Find the asymptotic variance of the MLE of $\theta$ based on this sample.
%===========================
% solution: 
%The log-likelihood is 
%$$
%\ln p(x\mid \theta) = \ln \Gamma(2\theta) - 2 \ln \Gamma(\theta) + (\theta - 1) \ln[x(1-x)]
%$$
%Then,
%$$
%\frac{\partial \ln p(x\mid \theta)}{\partial \theta} = \frac{2 \Gamma'(2\theta)}{\Gamma(2\theta)} - \frac{2 \Gamma'(\theta)}{\Gamma(\theta)} + \ln[x(1-x)]
%$$
%and
%\begin{align*}
%\frac{\partial^2 \ln p(x\mid \theta)}{\partial \theta^2} 
%&= \frac{4 \Gamma''(2\theta) \Gamma(2\theta) - 4 \Gamma'(2\theta)^2}{\Gamma(2\theta)^2} - \frac{2 \Gamma''(\theta) \Gamma(\theta) - 2 \Gamma'(\theta)^2}{\Gamma(\theta)^2} 
%\end{align*}
%Therefore,
%$$
%\fisher(\theta) 
%= -\Exp \left[\frac{\partial^2 \ln p(x\mid \theta)}{\partial \theta^2}\right] 
%= \frac{2 \Gamma''(\theta) \Gamma(\theta) - 2 (\Gamma'(\theta))^2}{\Gamma^2(\theta)} - \frac{4 \Gamma''(2\theta) \Gamma(2\theta) - (2 \Gamma'(2\theta))^2}{\Gamma^2(2\theta)}
%$$
%The asymptotic variance of the MLE is $\frac{1}{n \fisher(\theta)}$.

\end{problemset}
\newpage
\chapter{ Model Evaluation, Selection, and Analysis}\label{chapter:model_eva_sel}
\begingroup
\hypersetup{
linkcolor=structurecolor,
linktoc=page,  % page: only the page will be colored; section, all, none etc
}
\minitoc \newpage
\endgroup

\section{Linear Model Evaluation and Hypothesis Test}
\lettrine{\color{caligraphcolor}A}
After we have estimated, trained or fitted a model, it is important to assess how well the model performs. To evaluate its performance, we need to select an appropriate metric. There are many ways to measure how well a model fits the data; however, no single metric is universally suitable for all scenarios, datasets, or models. In practice, the choice of evaluation metric should be guided by the specific context, data, and model being used. Although numerous metrics exist, they can generally be grouped based on their evaluation objectives.

\subsection{Goodness of Fit}\label{section:goodness_fit}

The \textit{goodness of fit (GOF)} is a statistical measure that quantifies the agreement between observed data and the predictions generated by a model. Specifically, it evaluates the discrepancy between the vector of observed values $\by=[y_1, y_2,\ldots,y_n]^\top$ and the corresponding fitted or predicted values $\widehatby=[\widehaty_1, \widehaty_2, \ldots, \widehaty_n]^\top$.

A variety of GOF metrics exist to assess this agreement, including commonly used ones such as the mean squared error (MSE), likelihood-based measures, and others. In many modeling contexts, parameter estimation involves optimizing (often maximizing or minimizing) a chosen goodness-of-fit criterion to achieve the best possible alignment between the model and the data.

Fitting a model to data entails using a structured mathematical representation to approximate complex real-world observations. While this process simplifies the data into a compact set of parameters, the predicted values rarely match the observed values exactly. A central question in model evaluation is therefore: how significant is the deviation between the model's predictions and the actual data? A smaller discrepancy indicates a better fit, whereas a larger deviation suggests a poorer fit. The concept of goodness of fit formalizes this evaluation, providing quantitative tools to assess how well a model captures the underlying patterns in the observed data.

\subsection*{Nested Models}
Two statistical models are said to be \textit{nested} if one model can be derived from the other by imposing constraints on its parameters.
More precisely, suppose we fit two  linear models, Model-I and Model-II, using the same dataset. If restricting some parameters of Model-II---typically by setting them to zero---results in Model-I, then we say that \textit{Model-I is nested within Model-II}.

\begin{example}[Nested model]
For example, we consider the linear predictor  equation for Model-I as:
$$
\eta = \beta_0 + \beta_1 x_1 + \beta_2 x_2 + \beta_3 x_3.
$$
For the Model-II, we assume that the linear prediction takes the form:
$$
\eta = \beta_0 + \beta_1 x_1.
$$
In this case, Model-II is a nested version of Model-I because it can be obtained by constraining $\beta_2 = \beta_3 = 0$ in Model-I.
\end{example}

Generally, the more parameters a model has, the better it can fit the observed data. In the extreme case, when the number of parameters equals the number of observations, the model can perfectly fit all data points. Such a model is called a \textit{saturated model} (also known as a \textit{full model or maximal model}). While a saturated model achieves perfect fit, it does not generalize well to new data---it simply memorizes the training data, leading to \textit{overfitting}.

By imposing constraints on the parameters of a saturated model---for instance, setting some of them to zero---we reduce the model's complexity. This results in a simpler model. Although such a model may fit the data less closely, it often generalizes better to new data. However, reducing the number of parameters too much can significantly harm the model's fit, so fewer parameters do not always lead to better performance.

A saturated model assigns a separate parameter to each observation, meaning that with $n$ samples, there are $n$ parameters. This allows the model to perfectly reproduce the observed responses. On the opposite end of the spectrum, we define the \textit{null model}, sometimes referred to as the \textit{``worst" model}, which contains only an intercept term or a bias term ($\beta_0$) and no predictor variables. The null model has the least flexibility and typically provides the poorest fit among non-saturated models.

\subsection*{Likelihood Ratio}
In the Gauss-Markov linear model, the likelihood essentially represents the joint probability of observing the sample data given a model. The larger the likelihood value, the better the model fits the sample. Consequently, we can compare the goodness of fit between two models by comparing their likelihood values. Let's denote the simpler model with fewer parameters as Model $s$, with its likelihood denoted as $\mathcalL_s$, and another more complex model with more parameters as Model $m$, with its likelihood denoted as $\mathcalL_m$.

In statistics, to determine which of two nested models fits the data better, we use the \textit{likelihood ratio (LR)}. The LR compares how well two nested models fit the same dataset. The formula for the likelihood ratio statistic is:
\begin{equation}
\LR = -2 \left( \frac{\mathcalL_s}{\mathcalL_m} \right).
\end{equation}
where $\mathcalL_m$ is the likelihood of the complex model, and $\mathcalL_s$ is the likelihood of the simple model. From the formula, we can see that the likelihood ratio is the ratio of the likelihoods of the two models. Typically, we do not directly use the ratio of the likelihoods but work with the log-likelihood ratio:
\begin{equation}\label{equation:llr_def}
\LLR = -2 \ln \left( \frac{\mathcalL_s}{\mathcalL_m} \right)  = 2 (\ln \mathcalL_m - \ln \mathcalL_s).
\end{equation}
After taking the logarithm, the calculation becomes the difference in the log-likelihoods of the two models, making it more convenient and computationally efficient.

However, not any two models can be compared using the likelihood ratio; certain conditions must be met:
\begin{enumerate}[(i)]
\item Both models must use the same dataset, with the same number of samples. Different datasets yield different likelihood values, making comparisons meaningless.
\item The models must be nested.
\end{enumerate}

For two nested models, the primary distinction lies in the number of parameters. In linear models, this means the dimension of the parameter vector $\bbeta$ differs. The parameter vector $\bbeta_s$ of the simpler model is a subset of the parameter vector $\bbeta_m$ of the more complex model. Setting some elements of $\bbeta_m$ to zero yields $\bbeta_s$. Given the same fitting effect, a model with fewer parameters is generally preferred because it is simpler. However, theoretically, a model with more parameters will always fit the data at least as well as a simpler model. Thus, the log-likelihood of a complex model is always greater than or equal to that of a simpler model, ensuring LLR$\geq 0$. When LLR=0, it indicates that both models fit the data equally well, but in practice, this scenario is highly unlikely, and the log-likelihood ratio is usually positive.

A natural question arises: what range of LR values suggests that the fitting degrees of the two models are close? To address this, we need a method and standard for judgment. The likelihood-ratio statistic asymptotically follows a Chi-squared distribution, with degrees of freedom equal to the difference in the number of parameters between the two nested models. Since the LLR value is a random variable, directly using it to judge model fit is unreliable. Instead, hypothesis testing methods can be employed to assess the significance of the likelihood-ratio statistic. This test is known as the \textit{likelihood-ratio test (LRT)} or \textit{likelihood-ratio Chi-squared test (LRCT)}; see Section~\ref{section:mklm_hypotest}. In statistics, the LRT is a widely used approach for comparing the goodness of fit of two nested models based on maximum likelihood estimation.

\subsection*{Deviance}
The likelihood ratio test is a widely used method for comparing nested models. This test evaluates two models against each other rather than assessing a single model in isolation. Additionally, we introduce a derived statistic from the likelihood ratio statistic: the \textit{deviance statistic}. The deviance statistic is essentially a form of the likelihood ratio statistic but can be applied to measure the goodness of fit for a single model.

When developing a model, our goal is to ensure that the predicted or fitted values $\widehatby$ closely approximate the true data values $\by$. As mentioned previously, for an observational sample of size $n$, models can vary in complexity from having just one parameter (the null model) up to $n$ parameters (the saturated model). The simplest model with only one parameter (the null model) makes the same prediction for all samples, lacking any fitting ability. In contrast, the most complex model with $n$ parameters (the saturated model) can perfectly fit all samples but fails to generalize to new data. Although the saturated model cannot be directly used due to its lack of generalization, it serves as a benchmark for evaluating the fitting ability of other models.

We refer to the trained or fitted model as Model $\mathcalL_t$, with its likelihood denoted by $\mathcalL_t$. Similarly, let $\mathcalL_f$ represent the likelihood of the corresponding saturated model. The log-likelihood ratio statistic between these two models, known as the \textit{deviance} or \textit{deviance statistic}, is given by:
\begin{equation}\label{equation:def_deviance}
\Deviance = 2 (\ln \mathcalL_f - \ln \mathcalL_t),
\end{equation}
which has the same form as the log-likelihood ratio defined in  \eqref{equation:llr_def}, but uses different likelihood definitions.

\paragrapharrow{Deviance and MLE.}
The deviance statistic represents twice the difference between the log-likelihood of the saturated model and the fitted model. Since the dataset is observed, the log-likelihood of the saturated model is a constant. Therefore, during parameter estimation, minimizing the deviance is equivalent to maximizing the log-likelihood of the fitted model. Consequently, the maximum likelihood estimate (MLE) is also the minimum deviance estimate.

\paragrapharrow{Deviance and squared error.}
For the Gauss-Markov linear regression model in \eqref{equation:likelihood-of-gaussiannoise},  we have $\widehaty_i = \widehatbbeta^\top\bx_i$ for each observation $i$. Therefore, the deviance is:
\begin{equation}\label{equation:dev_squa_error}
\begin{aligned}
\Deviance = 
\sum_{i=1}^{n}\frac{1}{2\sigma^2} (y_i - \widehatbbeta^\top\bx_i)^2 - \frac{1}{2\sigma^2}(y_i - y_i)^2 = \sum_{i=1}^{n}\frac{1}{2\sigma^2} (y_i - \widehaty_i)^2.
\end{aligned}
\end{equation}
It can be seen that for the Gauss-Markov linear model, the deviance is consistent with the sum of squared errors. Indeed, the deviance can be viewed as an extension of the least squares method (with squared loss).

\index{Residual sum of squares}
\index{Total sum of squares}
\subsection{Coefficient of Determination $R^2$}\label{section:coeff_det_r2}

In classical linear regression models (also called OLS, Gauss-Markov linear models), a commonly used measure of goodness of fit is the \textit{$R^2$ statistic} (or $R^2$ measure, coefficient of determination $R^2$). 
The original definition of $R^2$ was introduced in the context of OLS models and does not directly apply to non-OLS models. 
Over time, many scholars have proposed various extensions of $R^2$ for use in other types of models, such as logistic regression for binary classification. 
In this section, we first provide the definition of $R^2$ in the context of OLS. 
A generalized version for generalized linear models (GLMs) will be discussed in Section~\ref{section:dev_r2}.

\subsection*{$R^2$ Measure}
The $R^2$, also known as the coefficient of determination, is a measure used to assess how well an OLS model fits the observed data.

\begin{definition}[$R^2$ measure]\label{definition:r2_mea}
Let $\by=[y_i]$ denote the observed values of $n$ samples, $\widehatby = [\widehat{y}_i]$ denote the predicted values of the model (i.e.,  $\widehat{\by} = \bX\widehat{\bbeta}=\bX (\bX^\top\bX)^{-1}\bX^\top\by$ is the projection of $\by$ onto the column space of $\bX$). 
Let $\overline{y} = \frac{1}{n}\sum_{i=1}^{n} y_i$ denote the sample mean of the observed responses. Then, the $R^2$ measure is defined as 
\begin{equation}
R^2 \triangleq  1 - \frac{\sum_{i=1}^{n}(y_i - \widehat{y}_i)^2}{\sum_{i=1}^{n}(y_i - \overline{y})^2}
=
1 - \frac{\normtwo{\by-\widehatby}^2}{ \normtwo{\by - \overline{y}\bone_n}^2},
\end{equation}
where $\bone_n$ denotes an all one's vector in  $\real^n$.
The term $\normtwo{\by-\widehatby}^2 =\normtwo{\be}^2$ is known as the \textit{residual sum of squares (RSS)}, and the term $\normtwo{\by - \overline{y}\bone_n}^2$ is called the \textit{total sum of squares (TSS)}:
\begin{equation}
	R^2 = 1 - \frac{\RSS}{\TSS}.
\end{equation}
The higher the value of $R^2$, the better the model fits the data.
\end{definition}

It's important to note that the RSS quantifies the model's fitting error---it is both the loss function and the objective function being minimized in the OLS framework (its sampling distribution under Gaussian disturbance is discussed in Section~\ref{sec:dist_sse}). 
A larger RSS indicates a worse fit, so the smaller the RSS, the better the model fits the data. The RSS can take any value in the interval [0, $\infty$]. While theoretically its minimum value is 0 (indicating a perfect fit), achieving this in practice is extremely rare; typically, a positive value is obtained.
Simply looking at an RSS value lacks a reference standard, making it difficult to directly judge whether the model's fitting ability has reached its limit or if there is still room for optimization. 

The denominator in the $R^2$ formula, $\sum_{i=1}^{n}(y_i - \overline{y})^2$, is known as the total sum of squares (TSS). It serves as a reference point for the RSS, representing the maximum possible value of RSS when using only an intercept (a constant prediction for all observations).

%To understand this better, consider the design matrix $\bX$ that includes a column of all ones (the intercept or bias term). This corresponds to the simplest possible model---one that predicts the same value (the sample mean $\overline{y}$) for every observation. This model is often referred to as the null model.
%
%To see this, consider the observed design matrix $\bX$ with a first column of all 1's, i.e., the \textit{bias term} or the \textit{intercept term}. Then, the column space of $\bX$ correspond to a ``line" in the original space (when the first column is not fixed to one's). 

To understand this better, consider the design matrix $\bX$ that contains only a column of all ones: $\bX_1 = \bone_n\in \real^n$ (i.e., the model  has only an intercept parameter, as mentioned above, such a model is called a null model). 
Then the least squares prediction becomes $[\overline{y},\overline{y}, \ldots, \overline{y}]^\top = \overline{y} \bone_n= \bX[\overline{y}, 0, \ldots, 0]^\top \in \cspace(\bX)$.
To be more concrete, if we only select the first column, we should project $\by$ onto the column space of $\bX_1$, and the hat matrix becomes $\bH_1 = \bX_1 (\bX_1^\top\bX_1)^{-1}\bX_1^\top = \frac{\bone_n \bone_n^\top}{n}$. 
Then,  the projection of $\by$ is $\overline{y} \bone_n = \bH_1\by$.
Therefore, the denominator part (the TSS) can be seen as the residual sum of squares of the null model.
Since the null model is the simplest possible model (predicting only the mean), its RSS serves as an upper bound for the RSS of any more complex model.

Therefore, in OLS regression, TSS acts a theoretical upper limit for RSS. 
As a result, the ratio $\frac{\RSS}{\TSS}$ lies within $[0, 1]$.  A value closer to 1 indicates a larger residual sum of squares, which implies a poorer model fit.  
Therefore, the model fit is inversely related to $\frac{\RSS}{\TSS}$ and directly related to the coefficient of determination, $R^2$.

To understand why comparing the use of $\bX$ versus $\bX_1$ in the model is significant, consider that when data can be accurately represented by a model, there should be a noticeable difference between a model with only the intercept term ($\bX_1$) and one incorporating all $p$ parameters ($\bX$). The coefficient of determination, $R^2$, ranges from 0 to 1.
When $R^2$ is close to 0, it indicates that the model  does not explain the variability in the data.
Conversely, when $R^2$ approaches 1, $\widehat{\by}$ closely matches $\by$, suggesting that the model effectively explains all variability, making it reasonable to use a linear model to describe the data. 
However, $\widehat{\by}$ represents the best possible fit under the given dataset for linear models. If $\widehat{\by}$ is far from $\by$, the data cannot be adequately represented by a linear model. In such cases, further variable selection procedures would not improve the model fit, and we will discuss these methods later.

\index{ANOVA decomposition}
\subsection*{ANOVA Decomposition}
By Pythagoras' theorem (Remark~\ref{remark:pythagoras}), we have the following property and definitions for the quantities in the $R^2$ measure:
\begin{subequations}\label{equation:anova-toy}
\begin{align}
\normtwo{\by - \overline{y} \bone_n}^2 &= \normtwo{\by - \widehat{\by}}^2 +\normtwo{\widehat{\by} - \overline{y}\bone_n}^2; \\
\text{Total sum of squares} &= \text{Residual sum of squares} + \text{Explained sum of squares}; \\
\TSS &= \RSS + \ESS.
\end{align}
\end{subequations}
This is an example of an \textit{ANOVA (short for analysis of variance) decomposition}. 
ANOVA decompositions partition variance (or sum of squares) into two or more components, which often exhibit orthogonality or adhere to the Pythagorean theorem.
To understand why this equality holds, let $\bH \triangleq \bX (\bX^\top\bX)^{-1}\bX^\top$ represent the hat matrix of the full design matrix $\bX$, and let $\bH_1 \triangleq \bX_1 (\bX_1^\top\bX_1)^{-1}\bX_1^\top$ be the hat matrix of the null model. 
We observe that:
\begin{subequations}
\begin{align}
\TSS &=\by^\top(\bI-\bH_1)^\top(\bI-\bH_1)\by;\\
\RSS &=\by^\top(\bI-\bH)^\top(\bI-\bH)\by;\\
\ESS &=\by^\top(\bH-\bH_1)^\top(\bH-\bH_1)\by.
\end{align}
\end{subequations}
It can be easily verified that $(\bI-\bH_1)$, $(\bI-\bH)$, and $(\bH-\bH_1)$ are symmetric and idempotent. 
Consequently, we have:
\begin{align*}
\by^\top(\bI-\bH_1)\by &= \by^\top(\bI-\bH)\by + \by^\top(\bH-\bH_1)\by
\qquad\implies\qquad 
\TSS = \RSS + \ESS.
\end{align*}

\subsection*{Variance Explanation}
Another interpretation of $R^2$ involves understanding it in terms of variance. 
As mentioned above, TSS represents the total sum of squares for the observed sample, while $\TSS-\RSS$ is known as the \textit{explained sum of squares (ESS)}. This ESS quantifies the amount of variance that the model can explain or fit:
\begin{equation}
	R^2 = 1 - \frac{\RSS}{\TSS} = \frac{\TSS - \RSS}{\TSS} = \frac{\ESS}{\TSS}.
\end{equation}
Therefore,  $R^2$ indicates the proportion of the total variance (TSS) of the observed variable explained by the model after adding predictive (feature) variables. 
For instance, if we calculate an $R^2 $ value of $0.75$, this means that the predictive (feature) variable $\bX$ explains 75\% of the variance in the observed response $\by$, leaving 25\%  unexplained by the current data $\bX$ or model. 
The unexplained portion corresponds to the RSS term.

Let $\mathcalV(\by)$ denote the variance of the observed sample, $\mathcalV(\widehatby)$ denote the variance of the fitted model, and $\mathcalV(\widehat{\bepsilon}) \triangleq \mathcalV(\by) - \mathcalV(\widehatby)$ represent the unexplained variance. $R^2$ can also be defined as follows:
\begin{equation}
R^2 = \frac{\mathcalV(\widehatby)}{\mathcalV(\by)} = \frac{\mathcalV(\widehatby)}{\mathcalV(\widehatby) + \mathcalV(\widehat{\bepsilon})}.
\end{equation}

\index{Degree of freedom}
\subsection*{Degree of Freedom}
In Remark~\ref{remark:df-of-error-vector}, we claimed that the residual vector $\be \triangleq \by - \widehat{\by}$ has $(n-p)$  degrees of freedom, and therefore, $\frac{\RSS}{n-p}$ adjusts for this by dividing the RSS by the appropriate number of degrees of freedom.

Let us define $\ba \triangleq \by - \overline{y} \bone_n$. We observe that the degree of freedom can be explained as $\ba^\top\bone = \sum_{i=1}^{n}a_i = 0$ such that the first $(n-1)$ elements can move freely in $\real^{(n-1)}$,  while the last element is determined as a linear combination of the others.
We formalize this observation in the following remark:
\begin{remark}[Degree of freedom of $\ba$]\label{remark:df_of_a}
The quantity $\frac{\TSS}{n-1} = \frac{\normtwo{\ba}^2}{n-1}$ adjusts for the degrees of freedom of $\ba$, while the degree of freedom of $\ba$ is $(n-1)$ if $\bX \in \real^{n\times p}$ has full column rank with $n\geq p$.

Similarly, let $\bb \triangleq \widehat{\by} - \overline{y}\bone_n$. 
Since $\bb = \ba- \be$,  it follows that the degrees of freedom associated with $\bb$ is $(p-1)$. 
And $\frac{\ESS}{p-1} = \frac{\normtwo{\bb}^2}{p-1}$ adjusts for the degrees of freedom of $\bb$. 
We summarize the the degrees of freedom for the three terms in Table~\ref{table:three-variation-df}.
\end{remark}

%	It is mysterious why we call $(n-p)$ as the degree of freedom of $\be$. The degree of freedom is the dimension of the space where a vector can stay, i.e., how freely a vector can move. Since $\be\in \real^n$ is perpendicular to the column space of $\bX$ as shown in Figure~\ref{fig:ls-geometric2}. That is, $\bX^\top \be = \bzero$ and $\be$ is in the null space of $\bX^\top$ which has dimension $p$. So $\be$ cannot move completely freely and loses $p$ degree of freedom.

\begin{table}[h!]
\centering
\setlength{\tabcolsep}{11.5pt}
\begin{tabular}{c|c|c}
\hline
& Notation                  & Degrees of freedom \\ \hline\hline
$\TSS$ & $\ba^\top \ba = (\by - \overline{y} \bone_n)^\top(\by - \overline{y} \bone_n)$        & $n-1$             \\ \hline
$\RSS$ & $\be^\top \be=(\by-\widehat{\by})^\top(\by - \widehat{\by})$                          & $n-p$             \\ \hline
$\ESS$ & $\bb^\top\bb=(\widehat{\by} - \overline{y}\bone_n)^\top(\widehat{\by} - \overline{y}\bone_n)$  & $p-1$             \\ \hline
\end{tabular}
\caption{Degrees of freedom in OLS.}\label{table:three-variation-df}
\end{table}
\index{Degree of freedom in OLS}
\index{ANOVA}

Furthermore, recall the decomposition:
$$
\begin{aligned}
\TSS &= \RSS + \ESS;\\
\by^\top(\bI-\bH_1)\by &= \by^\top(\bI-\bH)\by + \by^\top(\bH-\bH_1)\by,
\end{aligned}
$$
where $\TSS, \RSS$, and $\ESS$ are all quadratic forms of $\by$ with different \textit{defining matrices}: $\bI-\bH_1, \bI-\bH$, and $\bH-\bH_1$, respectively. 
The following facts about quadratic forms are important (see \citet{rawlings2001applied, gut2009multivariate} for further discussion):
\begin{enumerate}[(i)]
\item  Any sum of squares can be expressed in the form $\by^\top \bA\by$, where $\bA$ is a  symmetric positive semidefinite  matrix.
\item  The degrees of freedom associated with any quadratic form is equal to the rank of the defining matrix $\bA$, which equals its trace when the matrix is idempotent. This aligns with our earlier discussion since: $\rank(\bI-\bH_1)=n-1$, $\rank(\bI-\bH)=n-p$, and $\rank(\bH-\bH_1)=p-1$.
\item  Two quadratic forms are orthogonal if the product of their defining matrices is the zero matrix $\bzero$.
\end{enumerate}

\index{Analysis of variance}
\index{ANOVA}
\index{ANOVA decomposition}
\index{Model selection}
\index{$R^2$ test}
\index{$R^2$ estimator}
\index{Adjusted $R^2$ estimator}
\subsection*{$R^2$ Estimator and Adjusted $R^2$ Estimator}
Previously, we define the $R^2$ measure as: 
$$
R^2 = \frac{\ESS}{\TSS} = 1-\frac{\RSS}{\TSS},
$$
which quantifies the proportion of variance in the response variable explained by the model. This definition is equivalent to:
\begin{equation}
R^2 = \frac{\ESS}{\TSS} = 1-\frac{\RSS/n}{\TSS/n} =1-\frac{\normtwo{\by - \widehat{\by}}^2/n}{\normtwo{\by - \overline{y} \bone_n}^2/n} = 1 - \frac{\widehat{\sigma}^2}{\widehat{\sigma}_y^2},
\end{equation}
where $\widehat{\sigma}^2$ is the maximum likelihood estimate of the noise variance $\sigma^2$, as shown in  Theorem~\ref{theorem:mle-gaussian}, and $\widehat{\sigma}_y^2$ is an estimate of $\{y_1, y_2, \ldots, y_n\}$, with $y_i$ denoting the $i$-th element of $\by$. Based on this, we can define the \textit{population quantity}:
$$
\rho^2 \triangleq 1 -\frac{\sigma^2}{\sigma^2_y},
$$
where $\sigma^2$ represents the noise variance in the linear model, and $\sigma^2_y$ is the variance of output variables $\rvy$. From the discussion on the  degree of freedom, we realize that both the two estimators are biased estimators.  
Thus, $R^2$ itself  is a ``{biased}'' estimator of $\rho^2$.

The corresponding unbiased estimates of $\sigma^2$ and $\sigma^2_y$ are 
\begin{equation}\label{equation:unbiased_sigmay}
\begin{aligned}
&\normtwo{\by - \widehat{\by}}^2/(n-p),  &\gap\gap \text{(see Lemma~\ref{theorem:ss-chisquare}) }\\ 
&\normtwo{\by - \overline{y} \bone_n}^2/(n-1), &\gap\gap \text{(see Table~\ref{table:three-variation-df} or discussion below)}
\end{aligned}
\end{equation}
respectively.
Using these unbiased variance estimates, we define the following adjusted version of $R^2$:
$$
\overline{R}^2 = 1-\frac{\normtwo{\by - \widehat{\by}}^2/(n-p)}{\normtwo{\by - \overline{y} \bone_n}^2/(n-1)} = 1- (1-R^2) \frac{n-1}{n-p}.
$$
This is known as the \textit{adjusted $R^2$} estimator of $\rho^2$ \citep{theil1961economic}.
Unlike $R^2$, the adjusted $R^2$ accounts for the number of predictors in the model and penalizes the inclusion of irrelevant or uninformative features. As a result, adding meaningless features may actually cause $\overline{R}^2$ to decrease, making it a more reliable metric for evaluating whether new features contribute meaningfully to the model.

\index{Unbiased estimator}
\index{Gaussian sampling}
\paragrapharrow{Unbiased estimator of $\sigma^2_y$ from Gaussian sampling.}
Suppose $\rvy =[\ry_1, \ry_2, \ldots, \ry_n]^\top$ and $\ry_1, \ry_2, \ldots, \ry_n$ are random variables from Gaussian distribution, i.e., $\ry_1, \ry_2, \ldots, \ry_n \sim \normal(\mu_y, \sigma^2_y)$. Let
$$
\overline{\ry} \triangleq \frac{1}{n}\sum_{i=1}^{n} \ry_i 
\qquad\text{and}\qquad
\rS_{\ry}^2 \triangleq \frac{1}{n-1}\normtwo{\rvy - \overline{\ry} \bone_n}^2.
$$
Then, we have the following result.
\begin{lemma}[Gaussian sampling]\label{lemma:gaussian-sampling}
\item 1). The sample mean satisfies $\overline{\ry} \sim \normal(\mu_y, \sigma_y^2/n)$; 
\item 2). The random variable $\rS_{\ry}^2$ satisfies $\frac{n-1}{\sigma_y^2} \rS_{\ry}^2 \sim  \chi_{(n-1)}^2$;
\item 3). The random variables $\overline{\ry}$ and $\rS_{\ry}^2$ are independent.
\end{lemma}
The proof can be found in \citet{panaretos2016statistics}. From this lemma, we have $\Exp[\rS_{\ry}^2] = \sigma_y^2$ (by Definition~\ref{definition:chisquare_dist}) such that  $\rS_{\ry}^2$ is an unbiased estimator of $\sigma_y^2$, as claimed in \eqref{equation:unbiased_sigmay}.

\index{Unbiased estimator}

\index{ANOVA decomposition}
\subsection*{ANOVA Decomposition}
We now present distributional results related to the ANOVA decomposition.

\begin{theoremHigh}[Distribution results for ANOVA decomposition]\label{theorem:anova-r2-test}
Let $\rvy = \bX\bbeta + \epsilon$, where $\epsilon \sim \normal(\bzero, \sigma^2 \bI)$. 
%And assume $\bX$ is fixed and full rank.
And assume $\bX\in \real^{n\times p}$ is fixed and has full rank with $n\geq p$ (i.e., rank is $p$ so that $\bX^\top \bX$ is invertible). 
Then, under the null hypothesis $\mathcalH_0: \bbeta = [\beta_0, 0, 0, \ldots, 0]^\top$, we have the following results:
\begin{enumerate}[(i)]
\item  $\frac{1}{\sigma^2}\TSS\sim \chi_{(n-1)}^2$;
\item  $\frac{1}{\sigma^2}\RSS\sim \chi^2_{(n-p)}$;
\item  $\frac{1}{\sigma^2}\ESS\sim \chi_{(p-1)}^2$;
\item  $\RSS$ and $\ESS$ are independent.
\end{enumerate}
\end{theoremHigh}
\begin{proof}[of Theorem~\ref{theorem:anova-r2-test}]
\textbf{(i).}
Let $\bH_1= \frac{\bone_n \bone_n^\top}{n}$, then $\rvy - \overline{\ry} \bone_n = (\bI-\bH_1)\rvy$. 
Recall that $\rvy \sim \normal(\bX\bbeta, \sigma^2\bI)$, and it can be easily verified that $(\bI-\bH_1)$ is symmetric and idempotent (i.e., an orthogonal projection).
Then, by affine transformation of $\rvy$, it follows that 
$$
\begin{aligned}
\rvy - \overline{\ry} \bone_n
&\sim \normal\left((\bI-\bH_1)\bX\bbeta, (\bI-\bH_1)(\sigma^2\bI)(\bI-\bH_1)^\top\right) \\
&\stackrel{\dag}{=}\normal\left((\bI-\bH_1)\bX\bbeta, \sigma^2(\bI-\bH_1)^2\right)
\stackrel{\ddag}{=}\normal\left((\bI-\bH_1)\beta_0 \bone_n, \sigma^2(\bI-\bH_1)^2\right) \\
&=\normal\left(\bzero, \sigma^2(\bI-\bH_1)^2\right),
\end{aligned}
$$
where the equality ($\dag$) follows from the symmetry of $\bI-\bH_1$, and the equality ($\ddag$) follows from the hypothesis $\mathcalH_0$. 
The affine transformation of multivariate normal distribution (Lemma~\ref{lemma:affine_mult_gauss}) implies that $\rvy - \overline{\ry} \bone_n$ and $(\bI-\bH_1)\bepsilon$ have the same distribution. That is, $\TSS=\normtwo{\rvy - \overline{\ry} \bone_n}^2$ and $\bepsilon^\top (\bI-\bH_1)\bepsilon$ have the same distribution:
\begin{equation}
\TSS  \stackrel{d}{=} \bepsilon^\top(\bI-\bH_1)\bepsilon. \nonumber
\end{equation}
By Spectral Theorem~\ref{theorem:spectral_theorem} and Lemma~\ref{proposition:eigenvalues-of-projection} (the only possible eigenvalues of orthogonal projection matrices are 0 and 1), we can rewrite the sum of squared errors as $\TSS \stackrel{d}{=} \bepsilon^\top(\bI-\bH_1)\bepsilon = \bepsilon^\top(\bQ \bLambda\bQ^\top)\bepsilon$, where $\bI-\bH_1=\bQ \bLambda\bQ^\top$ is the spectral decomposition of $\bI-\bH_1$. 
By the fact that rotations on the normal distribution do not affect the distribution~(Lemma~\ref{lemma:rotat_multi_gauss}), we have
$$
\boldeta \triangleq \bQ^\top\bepsilon \sim \normal(\bzero, \sigma^2\bI)
\qquad\implies\qquad 
\TSS \stackrel{d}{=}  \boldeta^\top \bLambda \boldeta \sim \sigma^2 \chi^2_{\mathrm{rank(\bI-\bH_1)}}\sim\sigma^2 \chi^2_{(n-1)},
$$
where $\rank(\bI-\bH_1) = \trace(\bI) - \trace(\bH_1)=n-1$ by Lemma~\ref{lemma:rank-of-symmetric-idempotent}.

\paragraph{(ii).}
For $\RSS$, we realize that $\RSS = \normtwo{\rvy - \widehat{\rvy}}^2$ is equivalent to $\rve^\top\rve$ in Theorem~\ref{theorem:ss-chisquare}, and it follows:
$$
\RSS\sim \sigma^2 \chi^2_{(n-p)}.
$$

\paragraph{(iii).}
For $\ESS$, let $\bH = \bX (\bX^\top\bX)^{-1}\bX^\top$. Similarly, it holds that $\widehat{\rvy} - \overline{\ry}\bone_n = (\bH-\bH_1)\rvy$, whence we have 
$$
\begin{aligned}
\widehat{\rvy} - \overline{\ry}\bone_n
&\sim\normal\left((\bH-\bH_1)\bX\bbeta, (\bH-\bH_1)(\sigma^2\bI)(\bH-\bH_1)^\top\right) \\
&=\normal\left((\bH-\bH_1)\bX\bbeta, \sigma^2(\bH-\bH_1)^2\right)
\stackrel{\dag}{=}\normal\left((\bH-\bH_1)\beta_0 \bone_n, \sigma^2(\bH-\bH_1)^2\right) \\
&=\normal\left(\bzero, \sigma^2(\bH-\bH_1)^2\right). 
\end{aligned}
$$
where again the equality $(\dag)$ follows from the hypothesis $\mathcalH_0$, and the last equality follows from the fact that $\bH\bone_n=\bH_1\bone_n = \bone_n$.
Again, the affine transformation of multivariate normal distribution implies $\widehat{\rvy} - \overline{\ry}\bone_n$ and $(\bH-\bH_1)\bepsilon$ have the same distribution. This results in
\begin{equation}
\ESS =\normtwo{\widehat{\rvy} - \overline{\ry}\bone_n}^2  \stackrel{d}{=} \bepsilon^\top(\bH-\bH_1)\bepsilon. \nonumber
\end{equation}
Thus, for the spectral decomposition of $\bH-\bH_1 = \bQ\bLambda\bQ^\top$ and $\boldeta = \bQ\bepsilon$, we have
\begin{equation}
\ESS \stackrel{d}{=}  \boldeta^\top \bLambda \boldeta \sim \sigma^2 \chi^2_{\mathrm{rank(\bH-\bH_1)}}\sim\sigma^2 \chi^2_{(p-1)}, \nonumber
\end{equation}
where 
$
\rank(\bH-\bH_1)= \trace(\bH) - \trace(\bH_1)
=\trace(\bX (\bX^\top\bX)^{-1}\bX^\top)-\trace\left(\frac{\bone_n \bone_n^\top}{n}\right)
=p-1.
$

\paragraph{(iv).}
Finally, we have 
$$
\begin{aligned}
\Cov[(\bI-\bH)\rvy, (\bH-\bH_1)\rvy] 
&= (\bI-\bH)\Cov[\rvy,\rvy] (\bH-\bH_1)^\top
= \sigma^2 (\bI-\bH) (\bH-\bH_1) \\
&= \sigma^2 +\bH - \bH_1 - \bH +\bH\bH_1 
= \bzero,
\end{aligned}
$$
where last equality follows from Proposition~\ref{proposition:nested-projection} that $\bH\bH_1 = \bH_1 = \bH_1\bH$. This implies $\RSS$ and $\ESS$ are independent, from which the results follow.
\end{proof}

Then, combining Theorem~\ref{theorem:anova-r2-test}, under the null hypothesis $\mathcalH_0: \bbeta = [\beta_0, 0, 0, \ldots, 0]^\top$, we conclude that 
$$
\rT=\frac{\frac{1}{n-p}\RSS}{\frac{1}{p-1}\ESS  } \sim F_{n-p,p-1},
$$
which is independent of $\sigma^2$ and is also known as the \textit{test statistic for the $F$-test}.

Now suppose we are given a dataset $(\bx_1, y_1), (\bx_2, y_2), \ldots, (\bx_n, y_n)$, and we observe a specific value  $\rT=t$ for this data. Then the value
$$
\widetildep=\Pr(\rT((\bx_1, y_1), (\bx_2, y_2), \ldots, (\bx_n, y_n)) \geq t) = \Pr(F_{n-p,p-1} \geq t),
$$
is known as the \textit{$p$-value}. We reject the hypothesis $\mathcalH_0$ if $\widetildep<\alpha$, given some small $\alpha$, say $\alpha=0.05$.

A moment of reflection would reveal that the $p$-value  is the chance of extreme cases when the hypothesis is true. If the $p$-value is small, then the hypothesis has a low  probability of being correct and the observed data is sufficiently unlikely under the hypothesis, we reject. If not, we fail to reject (note that not rejecting is not equal to accepting the hypothesis.). We will provide an example of how this $F$-test works in Figure~\ref{fig:f-dist-example}  later.

\subsection{Hypothesis Tests}\label{section:mklm_hypotest}
When we train a model using sample data, compute goodness-of-fit metrics, and draw conclusions about the model's performance, we must keep in mind that these results are based on random samples. As such, the goodness-of-fit metrics themselves are random variables, and our conclusions are derived from statistical inference---which does not always yield perfectly accurate results.
Therefore, it is essential to assess the reliability of our conclusions. This is precisely what statistical inference, and specifically hypothesis testing, aims to do \citep{panaretos2016statistics}.

In previous sections, we introduced common metrics used to evaluate model fit in linear models, including their definitions and computational methods. However, we did not yet explain how to use these metrics to make formal decisions or draw conclusions about the model. In this subsection, we explore how to interpret the values of these metrics to assess model quality and quantify the reliability of our conclusions.
One of the most commonly used tools in statistical inference is hypothesis testing. Among various testing procedures, the \textit{likelihood ratio test} and the \textit{Wald test} are two widely used approaches for evaluating linear models.

Suppose we wish to test the values of $\bbeta=\widehat{\bbeta}$, or more generally, whether $\bbeta$ takes on a specific value $\bb$. We consider the null and alternative hypotheses:
$$
\mathcalH_0: \bbeta=\bb 
\qquad\text{versus}\qquad
\mathcalH_0: \bbeta\neq \bb. 
$$
We can perform this test using either of the following two approaches:

\paragrapharrow{Wald test.}
An obvious candidate for a test statistic is the squared Mahalanobis distance of $\widehat{\bbeta}$ from $\bbeta$, otherwise known as the Wald statistic. Under $\mathcalH_0$,  by \eqref{equation:large_asy_tmp}, Theorem~\ref{theorem:samplding_dist_lse_gaussian} for Gauss-Markov models, or Theorem~\ref{theorem:larg-sample-hat-beta} for large samples, we have:
$$
\rW\triangleq (\widehat{\bbeta} - \bb)^\top \widehatbV^{-1} (\widehat{\bbeta} - \bb) \overset{a}{\sim} \chi^2_{(p)},
$$
where $\widehatbV=S^2 (\bX^\top\bX)^{-1}$ is the estimated covariance matrix of $\widehatbbeta$.
Thus, we reject $\mathcalH_0$ at significance level $\alpha$ if $\rW > \chi^2_{p,\alpha}$,
where $\chi^2_{\nu, \alpha}$ refers to the $\alpha$-th quantile of a Chi-squared distribution with $\nu$ degrees of freedom.
\paragrapharrow{Likelihood ratio test.}

An alternative is a likelihood ratio test. Define
$$
\Lambda \triangleq 2 \ln \left( \frac{\mathcalL(\widehat{\bbeta})}{\mathcalL(\bbeta)} \right) = 2 (\ell(\widehat{\bbeta}) - \ell(\bbeta)), 
$$
where $\mathcalL(\bbeta)$ denotes the likelihood function of $\bbeta$ under the observed data set.
To derive the distribution of $\Lambda$, we expand $\ell(\bbeta)$ around $\ell(\widehat{\bbeta})$ using a Taylor series:x
$$
\ell(\bbeta) \overset{a}{=} \ell(\widehat{\bbeta}) + (\bbeta - \widehat{\bbeta})^\top \frac{\ell(\widehat{\bbeta}) }{\partial \bbeta}- \frac{1}{2} (\bbeta - \widehat{\bbeta})^\top \widehatbV^{-1} (\bbeta - \widehat{\bbeta}).
$$
Since $\frac{\ell(\widehat{\bbeta}) }{\partial \bbeta} = \bzero$, the expression simplifies to:
$$
2 (\ell(\widehat{\bbeta}) - \ell(\bbeta)) \overset{a}{=} (\bbeta - \widehat{\bbeta})^\top \widehatbV^{-1} (\bbeta - \widehat{\bbeta}) \overset{a}{\sim} \chi^2_{(p)}.
$$
Under the null hypothesis $\mathcalH_0$, we have $\bbeta = \bb$, and thus:
$$
\Lambda = 2 (\ell(\widehat{\bbeta}) - \ell(\bb)) \overset{a}{\sim} \chi^2_{(p)}.
$$
Hence, we reject $\mathcalH_0$ at significance level $\alpha$ if $\Lambda > \chi^2_{p,\alpha}$.

\index{Model diagnostics}
\index{Residual analysis}
\section{Linear Model Diagnostics}\label{section:gau_diag}

The statistical perspective of the linear model offers a robust framework for diagnostics, providing a systematic approach to assess and validate the assumptions underlying the model. One of its main advantages lies in its ability to rigorously evaluate the appropriateness of the linear relationship between variables, ensuring that the model accurately reflects the data's inherent structure. By leveraging statistical tools such as residual analysis, this perspective enables practitioners to identify potential violations of key assumptions like linearity, homoskedasticity, and independence of errors. This diagnostic capability is crucial for enhancing model reliability and validity, allowing for informed adjustments and refinements that ultimately lead to more accurate predictions and insights. Through a thorough examination of residuals and their patterns, analysts can detect outliers, influential observations, and other anomalies that might otherwise go unnoticed, thereby fostering a deeper understanding of the data and the relationships it encapsulates.

\index{Statistical leverage score}
\index{Cook's distance}
\index{Leverage point}
\subsection{Statistical Leverage Scores}
One  advantage of the statistical perspective of the linear model is the ability to identify and assess \textit{influential data points} through the use of \textit{leverage scores} and related metrics. Leverage scores, derived from the diagonal elements of the projection matrix (or the hat matrix), provide a quantitative measure of how far an observation's predictor values are from the mean of the predictors. This is crucial because observations with high leverage have the potential to exert significant influence on the fitted regression line, potentially skewing the results if they are outliers or contain errors.

In the literature, the outliers and the influential observations may be considered separately. However, in a sense, they both have the  effect of ``pulling" the regression line (surface) toward them.
Normally,  an \textit{outlier} can be identified by comparing the individual residual $e_i=y_i-\widehaty_i$ to the average residual.
That is, outliers are points falling far from the cloud surrounding the regression line. 
More formally, we mentioned in Section~\ref{section:by-geometry-hat-matrix} that the diagonal values of the projection matrix or the hat matrix ${\bH = \bX(\bX^\top\bX)^{-1}\bX^\top}$ are called \textit{statistical leverage scores}.
These leverage scores have been used extensively in classical regression diagnostics to identify potential outliers by, e.g., flagging data points with leverage score greater than 2 or 3 times the average value in order to be investigated as errors or potential outliers.

By Theorem~\ref{theorem:samplding_dist_lse_gaussian}, for the $i$-th observation $(\bx_i, \ry_i)$, we have 
\begin{equation}\label{equation:leve_var_ei}
\Var[\ry_i - \widehat{\ry}_i] = \Var[\re_i] = \sigma^2 (1 - h_{ii}).
\end{equation}
If $h_{ii} \approx 1$ (the $i$-th leverage score), then the model is constrained so $\widehat{\ry}_i = \bx_i^\top \widehatbbeta \approx \ry_i$.

However, since $\sum_{j=1}^{n} h_{ii} \equiv \trace(\bH) = \trace(\bX(\bX^\top\bX)^{-1}\bX^\top)=\trace((\bX^\top\bX)^{-1}\bX^\top\bX)=p$, it is not possible for all data points to have low leverage scores.
A balanced distribution of leverage would correspond to $h_{ii} = p/n$ for each observation $i$.
Therefore, the assumption $\max_{j \leq n} h_{ii} \xrightarrow{n \to \infty} 0$ is satisfied in Theorem~\ref{theorem:larg-sample-hat-beta}.
In practice,  if a leverage score satisfies $h_{ii} > 2p/n$ or $h_{ii} > 3p/n$, 
the corresponding observation is considered a \textit{leverage point}, and the model should be examined more closely---such as by refitting the model without the $i$-th observation to assess its influence on the results.

By examining these leverage scores, analysts can pinpoint data points that may require further investigation, ensuring that any anomalies do not unduly affect the model's predictions. Moreover, the integration of Cook's distance---a metric that combines information about both the residual and leverage---enables a comprehensive evaluation of each observation's impact on the overall model fit.
To see this,  we drop out the $i$-th observation, i.e., the potential leverage point. 
And let $\widehat{\bbeta}_{-i}$ be the least squares estimate when model is fitted to data without observation $i$, and let $\widehat{\rvy}_{-i} = \bX\widehat{\bbeta}_{-i}$ be the corresponding fitted value.
Using the \textit{Sherman-Morrison formula} \eqref{equation:she_morr_for}, it can be shown that 
\begin{equation}\label{equation:cook_b_bi}
\widehatbbeta - \widehatbbeta_{-i} 
=
(\bX^\top\bX)^{-1}\bx_i \frac{\ry_i - \widehat{\ry_i}}{1-h_{ii}}
\qquad\implies\qquad 
\Exp[\widehatbbeta - \widehatbbeta_{-i} ] = \bzero;
\end{equation}
see Problem~\ref{prob:cook_b_bi}.
By \eqref{equation:leve_var_ei}, the covariance matrix is
\begin{equation}
\Cov\left[\widehatbbeta - \widehatbbeta_{-i} \right] = \frac{\sigma^2}{1-h_{ii}} (\bX^\top \bX)^{-1} \bx_i \bx_i^\top (\bX^\top \bX)^{-1}.
\end{equation}
The rank of $\Cov[\widehatbbeta - \widehatbbeta_{-i} ]$ is one for nonzero $\bx_i$. The only nonzero eigenvalue of $\Cov[\widehatbbeta - \widehatbbeta_{-i} ]$ is $\sigma^2 \bx_i^\top (\bX^\top \bX)^{-2} \bx_i / (1 - h_{ii})$ and its associated eigenvector is $(\bX^\top \bX)^{-1} \bx_i$. 
When we denote by $\mathcalV$ the one-dimensional subspace generated by $(\bX^\top \bX)^{-1} \bx_i$ of the $p$-dimensional Euclidean space, the subspace $\mathcalV$ is just a line along which the eigenvector $(\bX^\top \bX)^{-1} \bx_i$ lies, and each $\widehatbbeta - \widehatbbeta_{-i} $ has a distribution with which a random variable takes on values in the set $\mathcalV$ with probability one \citep{kim2017cautionary}.

In order to investigate the change in the value of $\widehatbbeta$ due to a deletion of observation $i$, \citet{cook1982residuals} introduced an influence measure/distance based on the geometry of confidence ellipsoids as follows:
\begin{equation}\label{equation:cook_dist}
C_i \triangleq \frac{1}{p S^2} (\widehatbbeta - \widehatbbeta_{-i} )^\top (\bX^\top \bX) (\widehatbbeta - \widehatbbeta_{-i} )
=
\frac{1}{p S^2} (\widehat{\rvy} - \widehat{\rvy}_{-i})^\top (\widehat{\rvy} - \widehat{\rvy}_{-i})
= \frac{e_i^2 h_{ii}}{pS^2 (1 - h_{ii})^2}.
\end{equation}
where $S^2$ denotes the unbiased estimate of $\sigma^2$ (Theorem~\ref{theorem:ss-chisquare}). \textit{Cook's distance} thus measures scaled distance between $\widehat{\rvy}$ and $\widehat{\rvy}_{-i}$.
The last equality of \eqref{equation:cook_dist} shows that large $C_i$ implies large $e_i$ and large $h_{ii}$.
In practice,  a Cook's distance $C_i > 8/(n - 2p)$ worth a closer look, indicating that the $i$-th observation has a substantial influence on the model fit and deserves closer inspection.

\subsection{Gauss-Markov Assumptions}\label{section:gaus_mk_ana}

In Chapter~\ref{sec:lr-gaussian-noise}, we discussed that the Gauss-Markov linear model relies on four fundamental assumptions:
\begin{itemize}
\item \textit{Linearity.} The expected value of $\rvy$ is linearly related to the predictor matrix $\bX$, i.e, $\Exp[\rvy]=\bX\bbeta$.
\item \textit{Homoskedasticity.} The conditional variance of the errors is constant: $\Var[\epsilon_i \mid \bx_i]=\Exp[\epsilon_i^2 \mid \bx_i] = \sigma^2$  for all $i=1,2,\ldots,n$.
\item \textit{Gaussian distribution.} The errors are normally distributed: $\epsilon_i\sim \normal(0, \sigma^2)$.
\item \textit{Independence of errors.} The error terms $\epsilon_i$ and $\epsilon_j$ are independent for any $i\neq j$.
\end{itemize}

If any of these assumptions is clearly violated, then the Gauss-Markov model may no longer be appropriate for modeling the data. In such cases, it becomes important to evaluate whether there is evidence supporting or contradicting these assumptions.
The primary tool used to assess these assumptions is the residual vector, which captures the part of the response variable $\rvy$ that cannot be explained by the predictor variables in $\bX$. Recall that the residuals are given by:
$$
\rve = \rvy - \widehat{\rvy} = \rvy - \bX \widehatbbeta = (\bI - \bH)\rvy = (\bI - \bH)\bepsilon,
$$
where $\bH$ is the hat matrix.

Now, if the model is correctly specified and the errors follow a normal distribution, that is, $\bepsilon \sim \normal(\bzero, \sigma^2 \bI)$, then the residuals should also follow a normal distribution: 
$$
\rve \sim \normal(\bzero, \sigma^2(\bI - \bH));
$$ 
see Theorem~\ref{theorem:samplding_dist_lse_gaussian}.
Under this assumption, each residual random variable $\re_i$ has the distribution:
$$
\re_i \sim \normal\{0, \sigma^2(1 - h_{ii})\}
\qquad \text{and}\qquad 
\Cov[\re_i, \re_j] = -\sigma^2 h_{ij}.
$$
This means that the residuals are not independent---they are correlated---and they have unequal variances. 
To address this, we can attempt to ``decorrelate" the residuals using the spectral decomposition of the hat matrix.
We start by decorrelating 
$$
\be \leftarrow \bQ^\top \be 
\qquad \implies\qquad 
\be \sim \normal(\bzero, \sigma^2 (\bI-\bLambda)),
$$ 
where $\bH = \bQ \bLambda \bQ^\top$ denotes the spectral decomposition of $\bH$.
Then, we compute  the standardized residuals:
$$
e_i \triangleq \frac{e_i}{S\sqrt{1 - \lambda_{ii}}}, \ \forall\, i = 1,2, \ldots, n,
$$
where $S^2$ denotes the unbiased estimate of $\sigma^2$ (Theorem~\ref{theorem:ss-chisquare}).
These processed residuals are  uncorrelated and have variance $\approx 1$.

\paragrapharrow{Diagnostics for linearity.}
To check for linearity, a simple approach is to examine plots of the residual vector against each covariate. Under the assumption of linearity, we have $ \bX^\top \be = \bzero$ (Theorem~\ref{theorem:rank_def_ls_prop}), which implies that there should be no correlation between the covariates and the residuals.

Therefore, we plot the standardized residuals $\be$ against each covariate (i.e., the columns of $\bX$). These plots should not display any systematic patterns. The presence of a clear pattern suggests an incorrect specification of the relationship between the response and the corresponding covariate---for example, it may indicate the need to include a transformation of that explanatory variable in the model. As shown in Figure~\ref{fig:diag_linear_no}, a sinusoidal pattern might suggest a sine-like relationship between the covariate and the residuals.

In addition, these plots can aid in variable selection. We can also plot the standardized residuals $\be$ against variables that were not included in the model. In this case, no systematic pattern should be present either. If such a pattern appears, it indicates that an omitted variable may be important and should be considered for inclusion. Conversely, if no such pattern exists, the excluded variable may be irrelevant and could potentially be removed from consideration.

\begin{figure}[h]
\centering  
\vspace{-0.35cm} 
\subfigtopskip=2pt 
\subfigbottomskip=6pt 
\subfigcapskip=-2pt 
\subfigure[Linearity.]{\includegraphics[width=0.4\textwidth]{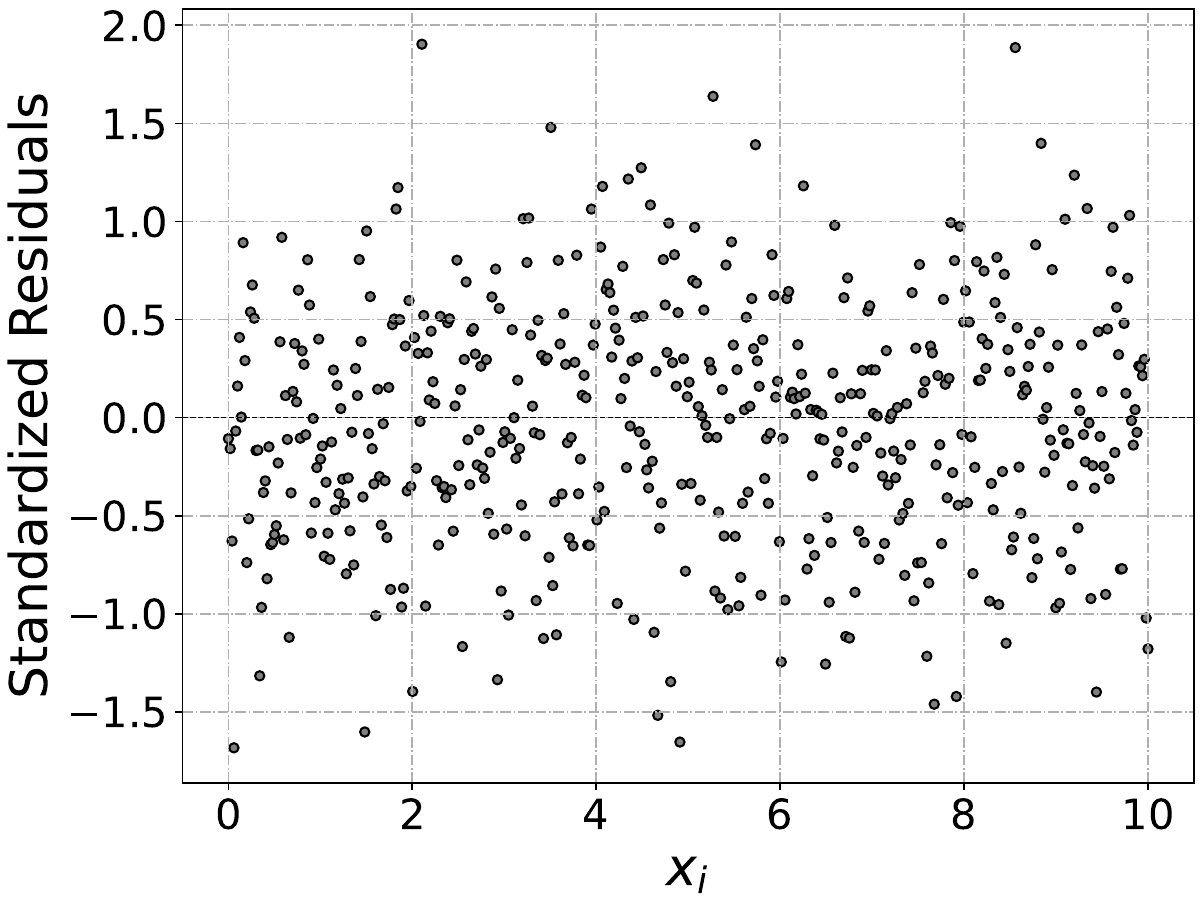} \label{fig:diag_linear_yes}}
\subfigure[Nonlinearity.]{\includegraphics[width=0.4\textwidth]{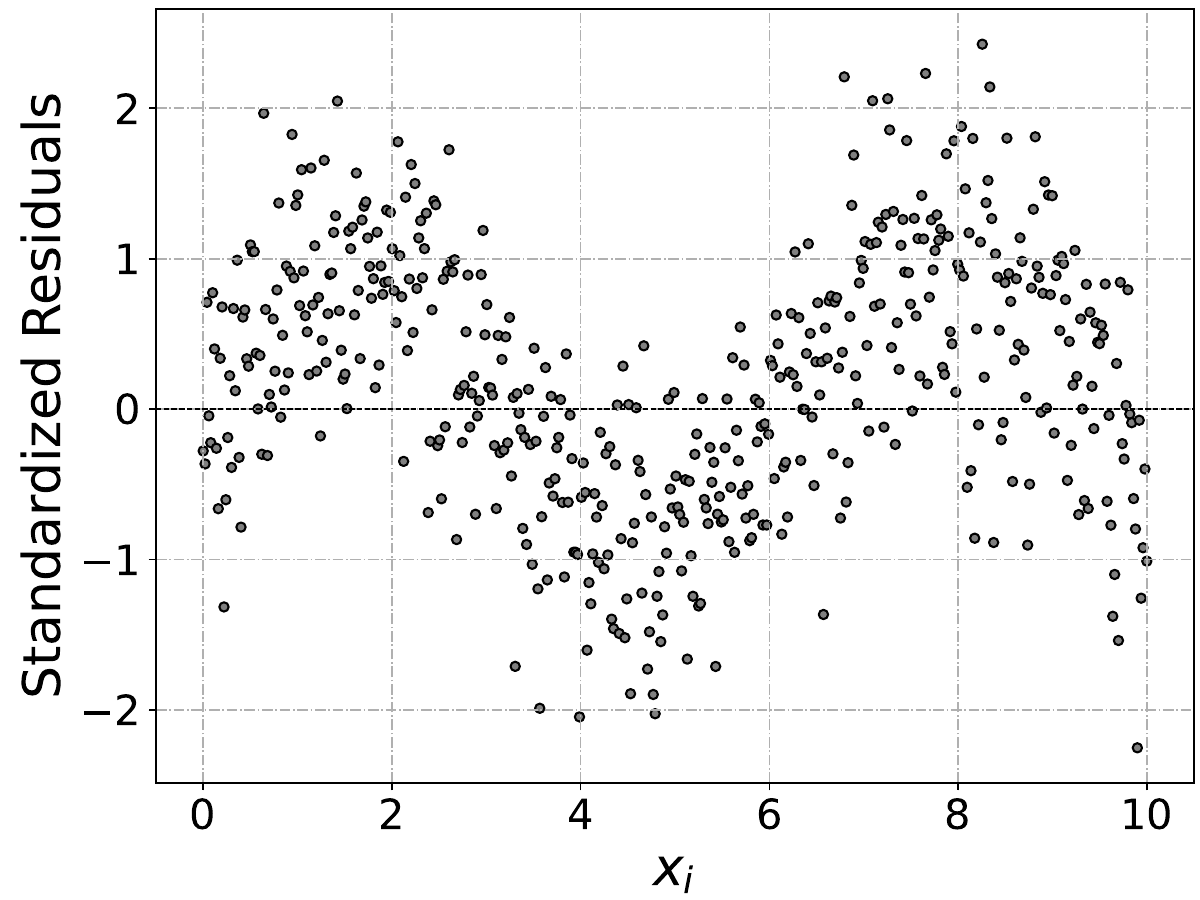} \label{fig:diag_linear_no}}
\caption{The relationship between the $i$-th covariate  and the residual.}
\label{fig:diag_linear}
\end{figure}

\paragrapharrow{Diagnostics for homoskedasticity.}
Under the assumption of homoskedasticity, the variance of the error terms $\epsilon_i$is constant across all observations:
$$
\Var[\epsilon_i] = \sigma^2, \ \forall\, i=1,2,\ldots,n.
$$
Since the residuals $\rve$ and the fitted values $\widehat{\rvy}$ are independent (Lemma~\ref{theorem:independence_samplding_dist_lse_gaussian}), we can plot $\be$ against  $\widehat{\rvy}$.

Ideally, the plot should show a random scatter of points with approximately constant vertical spread across all fitted values; see Figure~\ref{fig:diag_homo_yes}. A consistent spread indicates that the homoskedasticity assumption is likely satisfied.

However, the presence of a discernible pattern---such as a funnel shape or increasing/decreasing spread---suggests a violation of the homoskedasticity assumption, indicating heteroskedasticity; see Figure~\ref{fig:diag_homo_no}.

\begin{figure}[h]
\centering  
\vspace{-0.35cm} 
\subfigtopskip=2pt 
\subfigbottomskip=6pt 
\subfigcapskip=-2pt 
\subfigure[Homoskedastic.]{\includegraphics[width=0.4\textwidth]{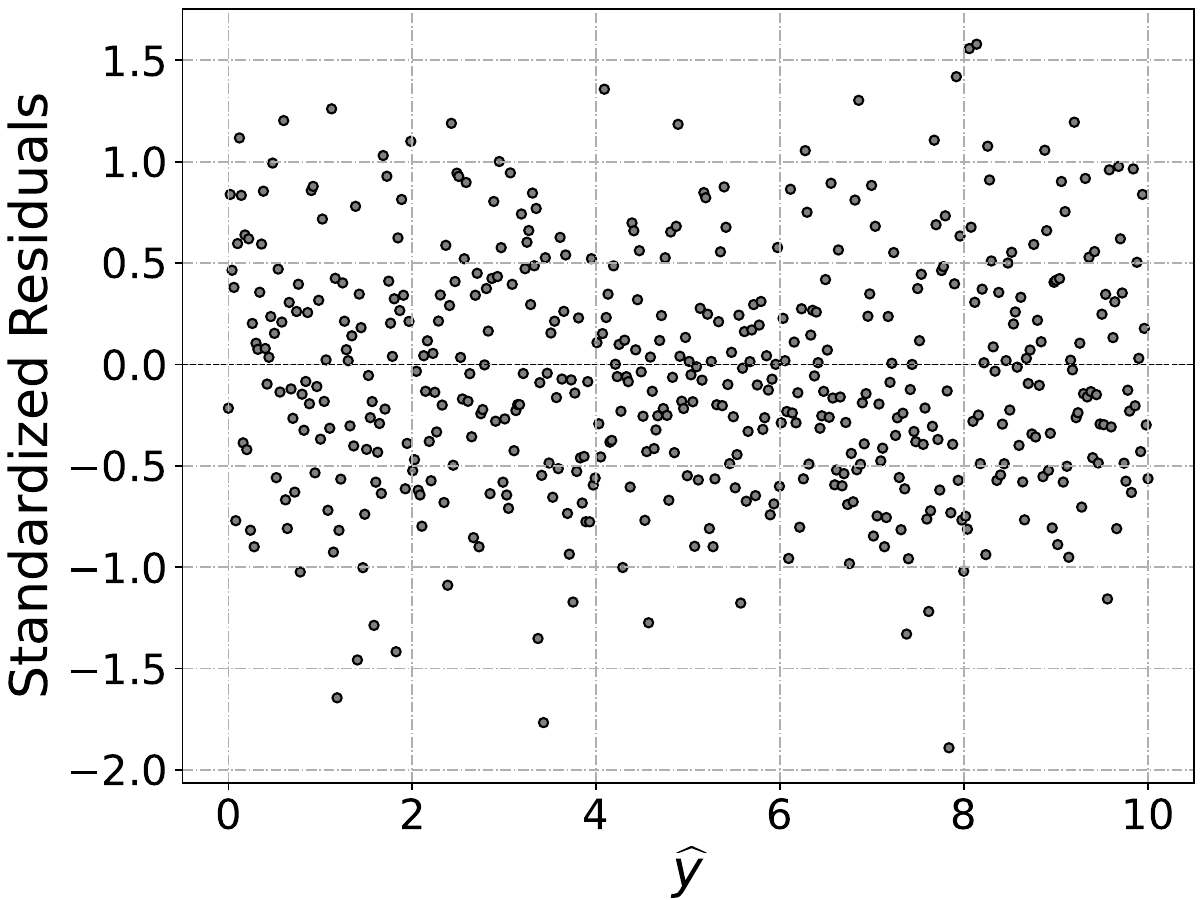} \label{fig:diag_homo_yes}}
\subfigure[Heteroskedastic.]{\includegraphics[width=0.4\textwidth]{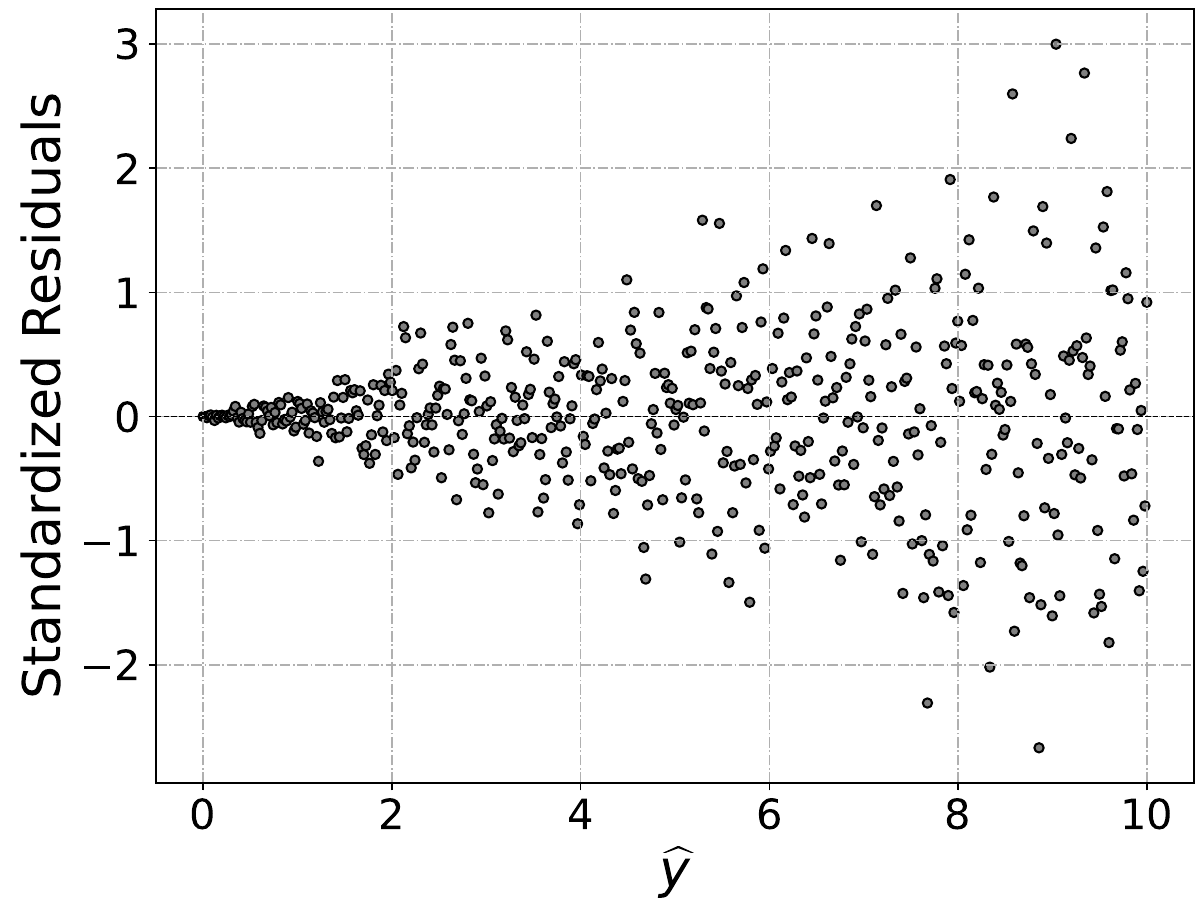} \label{fig:diag_homo_no}}
\caption{Residuals showing constant spread (homoskedastic) and varying spread (heteroskedastic).}
\label{fig:diag_homo}
\end{figure}

\begin{figure}[h]
\centering  
\vspace{-0.35cm} 
\subfigtopskip=2pt 
\subfigbottomskip=6pt 
\subfigcapskip=-2pt 
\subfigure[Gaussian.]{\includegraphics[width=0.4\textwidth]{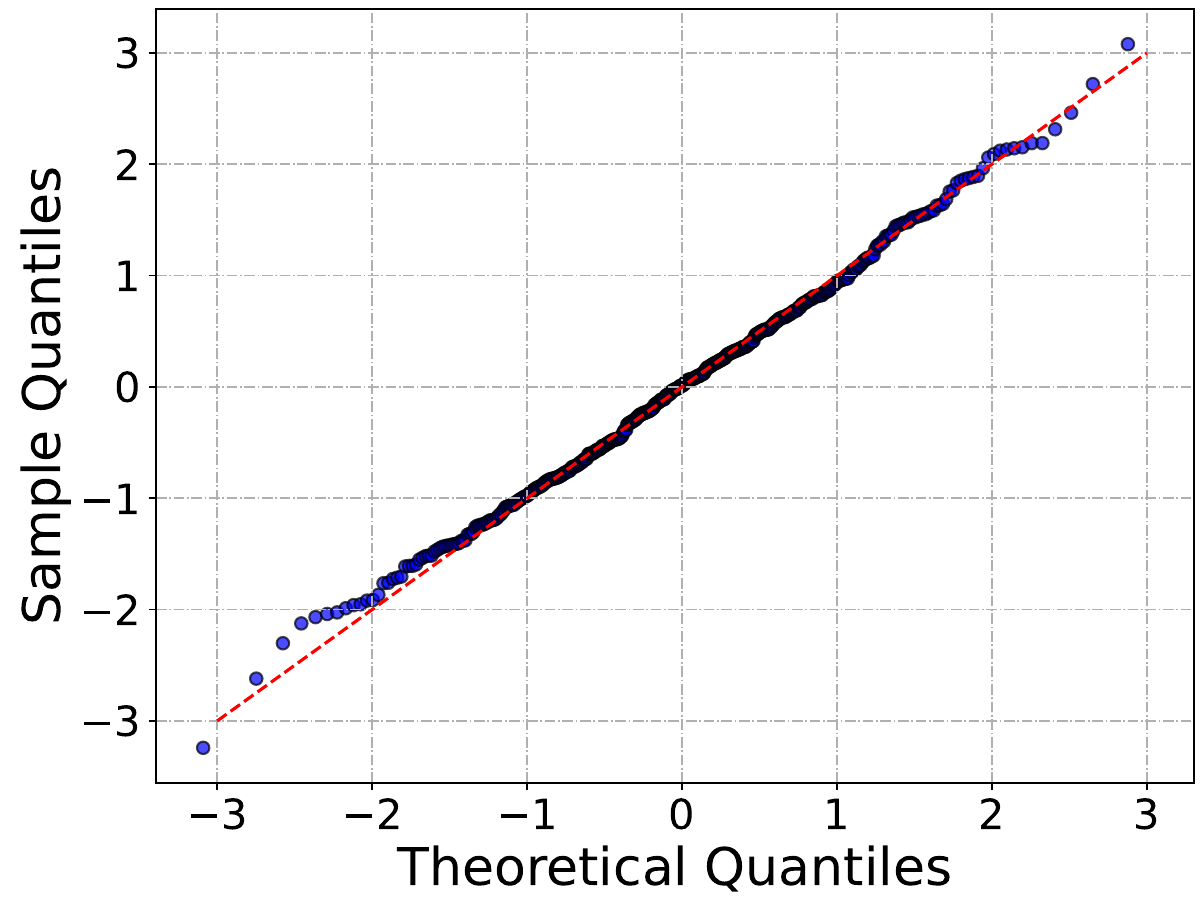} \label{fig:diag_qq_yes}}
\subfigure[Non-Gaussian.]{\includegraphics[width=0.4\textwidth]{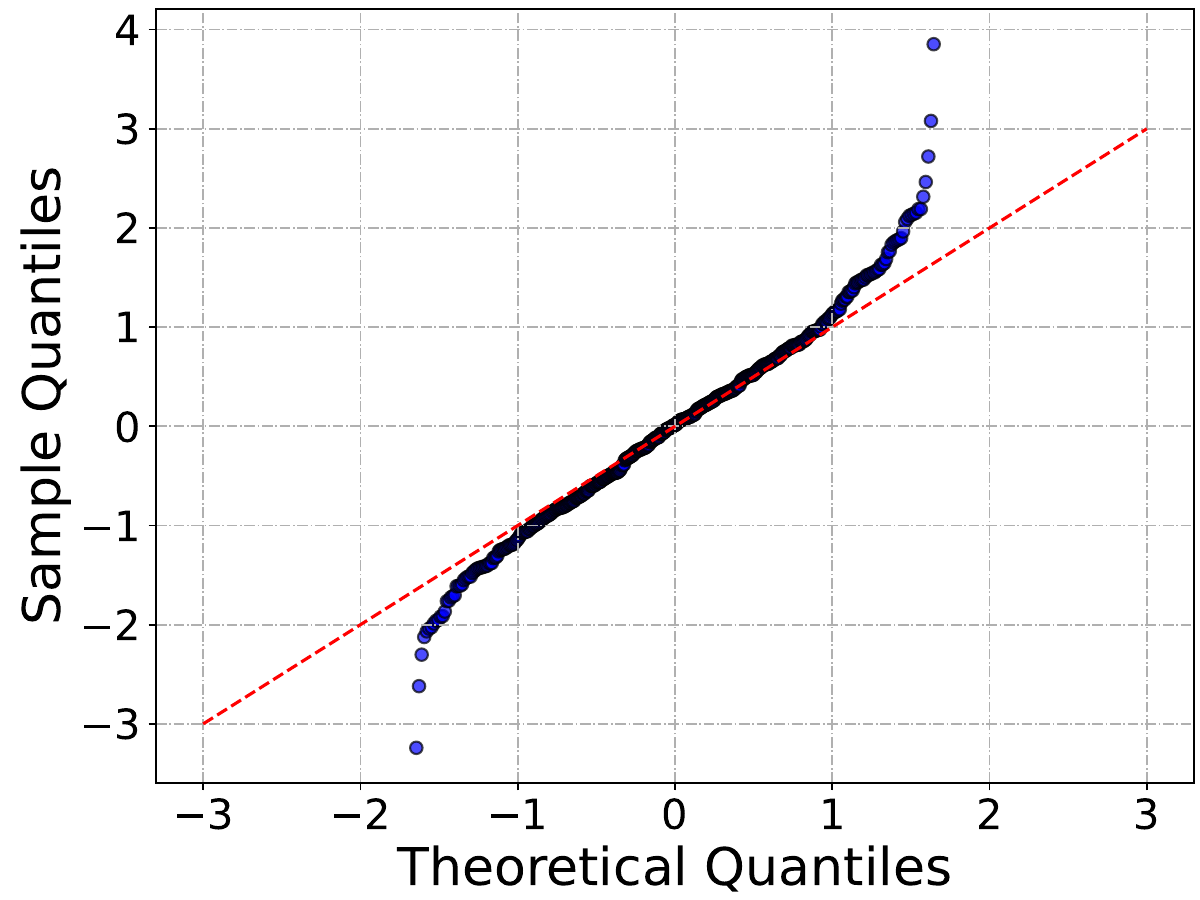} \label{fig:diag_qq_no}}
\caption{QQ plot for an empirical Gaussian and a non-Gaussian distribution.}
\label{fig:diag_qq}
\end{figure}

\index{QQ plot}
\paragrapharrow{Diagnostics for normality.}

A simple and effective method to assess normality is to compare the distribution of standardized residuals to a Gaussian distribution by using a \textit{quantile-quantile (QQ) plot}. This involves plotting the empirical quantiles of the residuals against the corresponding theoretical quantiles from a standard normal distribution.

Recall that the $\alpha$-quantile ($\alpha \in [0, 1]$) of a distribution $F$ is the value $F^{-1}(\alpha)$ defined as
$$ 
F^{-1}(\alpha) \triangleq \inf\{t \in \mathbb{R} : F(t) \geq \alpha\}. 
$$
Given a set of samples $y_1, y_2,\ldots, y_n$, the empirical $\alpha$-quantile is the value defined as
$$
\widehat{F}^{-1}(\alpha) \triangleq \inf\{t \in \mathbb{R} : \widehat{F}(t) \geq \alpha\} = \inf\left\{t \in \mathbb{R} : \frac{\#\{y_i \leq t\}}{n} \geq \alpha\right\},
$$
where $\widehat{F}$ denotes the empirical cumulative distribution function.

A QQ plot compares the empirical quantiles of a dataset to the theoretical quantiles of a reference distribution---in this case, the standard normal distribution. If the data are normally distributed, we expect the plotted points to approximately lie along a straight line at a $45^\circ$ angle.
To construct a QQ plot for standardized residuals:
\begin{itemize}
\item Sort the standardized residuals in ascending order: $ e_{(1)} \leq e_{(2)} \leq \cdots \leq e_{(n)} $~\footnote{$e_{(i)}$ denotes the $i$-th smallest value in $e_1, e_2, \ldots, e_n$.}.
\item Plot the empirical $\left(\frac{k}{n}\right)_{k=1}^n$ quantiles of standardized residuals $ e_{(1)} \leq e_{(2)} \leq \cdots \leq e_{(n)} $ against theoretical quantiles $\Phi^{-1}\left(\frac{1}{n+1}\right), \Phi^{-1}\left(\frac{2}{n+1}\right), \ldots, \Phi^{-1}\left(\frac{n}{n+1}\right)$ of a $\normal(0, 1)$ distribution, where $\Phi(y) = \int_{-\infty}^{y} \normal(u\mid 0,1)du= \frac{1}{\sqrt{2\pi}} \int_{-\infty}^{y} \exp(-\frac{u^2}{2}) du $ denotes the cumulative distribution function of $\normal(0,1)$.
 Note that here we use  $\Phi^{-1}\left(\frac{k}{n+1}\right)$ instead of $\Phi^{-1}\left(\frac{k}{n}\right)$ to account for the fact that the $k$-th order statistic is not exactly at the $\frac{k}{n}$-th quantile of the population.
\end{itemize}
If the residuals follow a normal distribution, the points in the QQ plot should lie close to the $45^\circ$ reference line. Substantial deviations from this line suggest departures from normality; see Figure~\ref{fig:diag_qq_no} for an example.

\index{Clustering}
\paragrapharrow{Diagnostics for independence.}
In practice, verifying the assumption of independent errors---i.e., that  $\Cov[\rve] = \sigma^2 \bI$---can be challenging. One common issue to look for is clustering in the residuals, which might suggest the presence of dependence among observations.

For example, if the data include groups of related individuals or repeated measurements, their responses may be correlated, leading to correlated residuals. While clustering and its implications are an important topic, a detailed discussion of clustering algorithms and models for dependent data is beyond the scope of this book and we will not provide the details.

\begin{figure}[h!]
\centering
\includegraphics[width=0.5\textwidth]{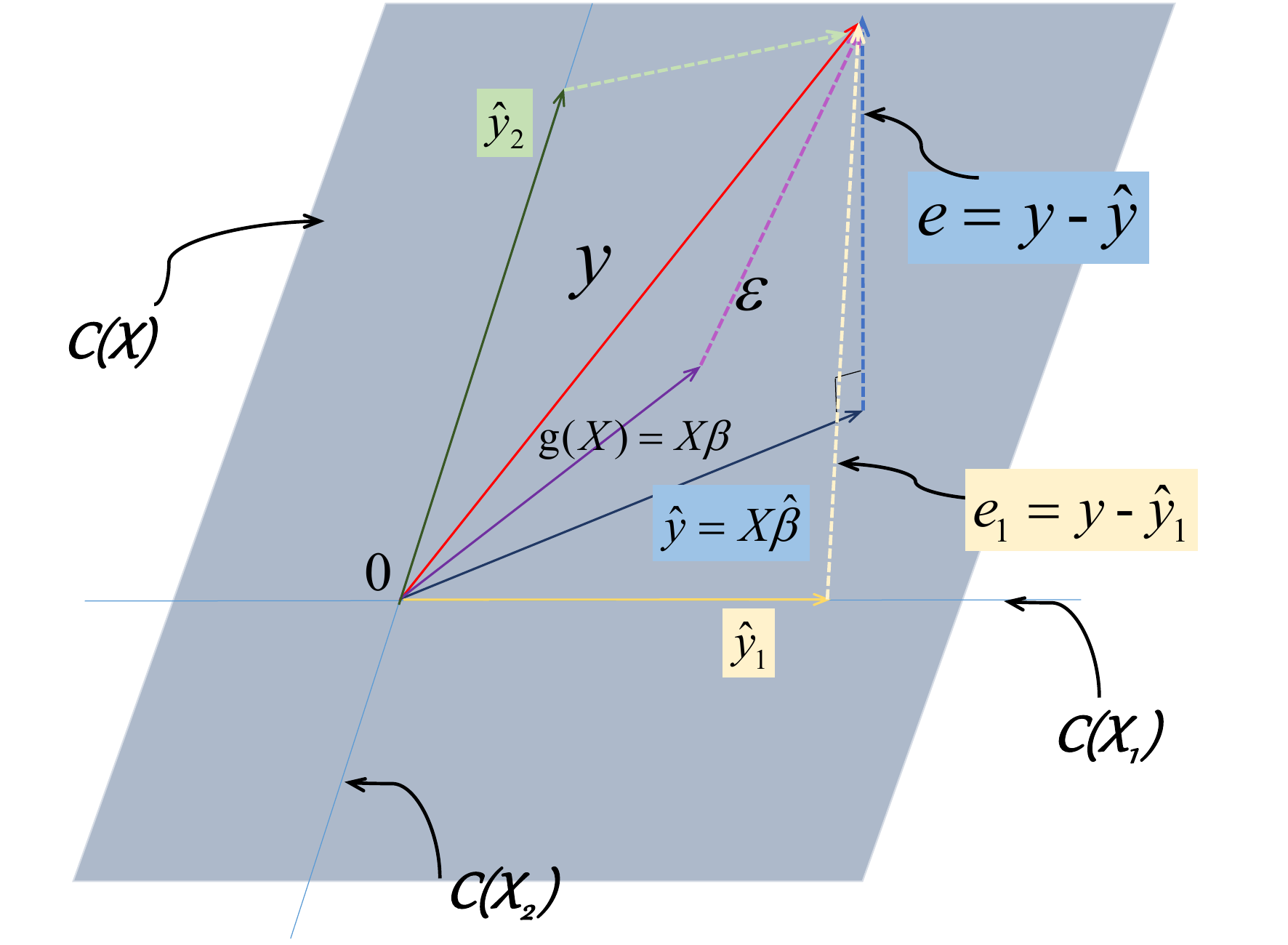}
\caption{Projection onto the hyperplane of $\cspace(\bX)$ and $\cspace(\bX_1)$.}
\label{fig:ls_geometric3}
\end{figure}
\index{Variable selection}
\section{Linear Model Variable Selection}\label{section:variable-selection}

In Section~\ref{section:gaus_mk_ana}, under the diagnostics for linearity, we discussed that a systematic pattern in the linearity plot (see, for example, Figure~\ref{fig:diag_linear_no}) indicates the omission of a necessary covariate. This suggests the need to include this variable in our model. Conversely, if no such pattern is observed, it may indicate that an irrelevant covariate can be removed from the model.

This observation can be integrated into a variable selection procedure. Furthermore, in this section, we illustrate how the ANOVA decomposition presented in Theorem~\ref{theorem:anova-r2-test} can be applied within the variable selection process.
Incorporating these insights allows for a more systematic approach to model refinement by ensuring that only relevant variables are retained, based on their contribution to explaining variance.

In Section~\ref{section:by-geometry-hat-matrix}, we showed that $\widehat{\rvy}$ is the projection of $\rvy$ onto the column space of $\bX$, denoted $\cspace(\bX)$. Assuming $n\geq p$, the matrix $\bX\in\real^{n\times p}$ has full rank, and thus the dimension of $\cspace(\bX)$ is $p$. 
If we decompose $\bX$ into two blocks
$$
\bX = [\bX_1, \bX_2],
$$
where $\bX_1\in \real^{n\times q}$ and $\bX_2\in \real^{n\times (p-q)}$ with $q<p$. That is, $\bX_1$ contains the first $q$ columns of $\bX$, and $\bX_2$ consists of the remaining columns (alternatively, we can choose $\bX_1$ as random $q$ columns from $\bX$, and $\bX_2$ as the rest).  
Then, the linear model can be expressed as:
$$
\rvy = \bX\bbeta +\bepsilon =
[\bX_1, \bX_2]
\begin{bmatrix}
\bbeta_1 \\
\bbeta_2
\end{bmatrix} +\bepsilon
=\bX_1\bbeta_1 + \bX_2\bbeta_2+\bepsilon.
$$
A natural question arises: Can we set $\bbeta_2 = \bzero$ without significantly increasing the reconstruction error compared to the case when $\bbeta_2\neq \bzero$? 
In other words, are the last $p-q$ variables redundant in the context of the least squares model?
Equivalently, can the submodel defined  solely by $\bbeta_1$ perform as well as the full model defined by  $\bbeta$?

If we consider only the first $q$ variables, the projection of $\rvy$ onto the column space of $\bX_1$ is achieved using the hat matrix $\bH_1 = \bX_1 (\bX_1^\top\bX_1)^{-1}\bX_1^\top$. 
Consequently,  the projection of $\rvy$ is $\widehat{\rvy}_1 = \bH_1\rvy$, and the corresponding error vector is $\rve_1 = \rvy - \widehat{\rvy}_1$. 
The scenario  is illustrated  in Figure~\ref{fig:ls_geometric3}.

For $\widehat{\bbeta}$, we define the residual sum of squared error by $\RSS(\widehat{\bbeta}) \triangleq \normtwo{\rvy - \widehat{\rvy}}^2 = \normtwo{\rve}^2$. Applying Pythagoras' theorem (Remark~\ref{remark:pythagoras}), we derive the following relationships; see also Figure~\ref{fig:ls_geometric3}:
\begin{subequations}
\begin{align}
\RSS(\widehat{\bbeta}_1) &= \RSS(\widehat{\bbeta}) + \normtwo{\rve_1 - \rve}^2; \\
\normtwo{\rve_1}^2 &= \normtwo{\rve}^2 + \normtwo{\rve_1 - \rve}^2;\\
\normtwo{\rvy - \widehat{\rvy}_1}^2 &= \normtwo{\rvy - \widehat{\rvy}}^2 + \normtwo{\widehat{\rvy} - \widehat{\rvy}_1}^2. \label{equation:var_sec_two}
\end{align}
\end{subequations}
This shows that $\RSS(\widehat{\bbeta}_1) \geq \RSS(\widehat{\bbeta})$. Therefore, to decide whether we can set $\bbeta_2=\bzero$ reduces to comparing how much larger $\RSS(\widehat{\bbeta}_1)$ is compared to $\RSS(\widehat{\bbeta})$.
\begin{theoremHigh}[Sampling distribution for variable selection]\label{theorem:anova-model-selection}
Let $\rvy = \bX\bbeta + \bepsilon$, where $\bepsilon \sim \normal(\bzero, \sigma^2 \bI)$. 
And assume that $\bX\in \real^{n\times p}$ is fixed and has full rank with $n\geq p $ (i.e., its rank is $p$ so that $\bX^\top \bX$ is invertible). 
Suppose we decompose the variables into two parts such that $\bX = [\bX_1, \bX_2]$, corresponding to parameters $\bbeta_1$ and $\bbeta_2$, respectively (with 
$\bbeta=\begin{bmatrixfoot}
\bbeta_1 \\
\bbeta_2
\end{bmatrixfoot} $). 
Then, the following results hold:
\begin{enumerate}[(i)]
\item  $(\rve - \rve_1 )\perp \rve$;
\item  $\RSS(\widehat{\bbeta})$ and $\RSS(\widehat{\bbeta}_1) - \RSS(\widehat{\bbeta})$ are independent;
\item  Under the null hypothesis $\mathcalH_0: \bbeta_2 = \bzero$, $\normtwo{\rve-\rve_1}^2\sim \sigma^2\chi_{(p-q)}^2$. 
\end{enumerate}
\end{theoremHigh}
\begin{proof}[of Theorem~\ref{theorem:anova-model-selection}]
\textbf{(i).} We begin by noting that $\rve - \rve_1 =- \widehat{\rvy}+ \widehat{\rvy}_1$. 
Since $\widehat{\rvy}\in \cspace(\bX)$ and $\widehat{\rvy}_1 \in \cspace(\bX_1)$, $\rve - \rve_1 \in \cspace(\bX)$. 
Additionally, $\rve$ is orthogonal to the column space of $\bX$. Therefore, $(\rve - \rve_1) \perp \rve$; see Figure~\ref{fig:ls_geometric3}.

\paragraph{(ii).}
From Proposition~\ref{proposition:nested-projection}, since $\cspace(\bX_1) \subseteq \cspace(\bX)$, we have $\bH_1=\bH_1\bH$. 
Then, $\rve_1 = (\bI-\bH_1)\rvy = (\bI-\bH_1\bH)\rvy$ . Therefore, 
$$
\begin{aligned}
\rve&= (\bI-\bH)\rvy ;\\
\rve - \rve_1 &= (\bI-\bH)\rvy - (\bI-\bH_1\bH)\rvy = (\bH_1-\bI)\bH\rvy.
\end{aligned}
$$
Recall that $\rvy \sim \normal(\bX\bbeta, \sigma^2\bI)$. Then, the covariance between $\rve$ and $\rve-\rve_1$ is given by:
$$
\begin{aligned}
\Cov[\rve, \rve-\rve_1] &= \Cov[(\bI-\bH)\rvy, (\bH_1-\bI)\bH\rvy] 
= (\bI-\bH)\Cov[\rvy, \rvy] \bH^\top (\bH_1-\bI)^\top \\
&= (\bI-\bH)(\sigma^2\bI) \bH (\bH_1-\bI) 
= (\bH - \bH^2) (\bH_1-\bI) 
=\bzero, 
\end{aligned} 
$$
where the second equality uses the identity $\Cov[\bA\rvx, \bB\rvy] = \bA\Cov[\rvx, \rvy]\bB^\top$ for deterministic  matrices $\bA$ and $\bB$, and the last equality follows from the idempotency of $\bH$. 
Thus, $\RSS(\widehat{\bbeta})$ and $\RSS(\widehat{\bbeta}_1) - \RSS(\widehat{\bbeta})$ are independent.

\paragraph{(iii).}
Next, consider the distribution of $\rve - \rve_1$:
$$
\begin{aligned}
\rve - \rve_1 &= (\bH_1-\bI)\bH\rvy
\sim \normal\left((\bH_1-\bI)\bH\bX\bbeta, (\bH_1-\bI)\bH(\sigma^2\bI)\bH^\top(\bH_1-\bI)^\top\right).
\end{aligned}
$$
Once again, from Proposition~\ref{proposition:nested-projection}, we know that $\bH\bH_1 = \bH_1 = \bH_1\bH$, so:
$$
\begin{aligned}
\rve - \rve_1 
&\sim \normal\left((\bH_1-\bH)\bX\bbeta, \sigma^2(\bH-\bH_1)\right)
\stackrel{\dag}{=} \normal\left((\bH_1-\bI)\bX\bbeta, \sigma^2(\bH-\bH_1)\right)\\
&= \normal\left((\bH_1-\bI)(\bX_1\bbeta_1 + \bX_2\bbeta_2), \sigma^2(\bH-\bH_1)\right) \\
%&= \normal\left((\bH_1-\bI)\bX_1\bbeta_1 + (\bH_1-\bI)\bX_2\bbeta_2, \sigma^2(\bH-\bH_1)\right) \\
&\stackrel{\ddag}{=} \normal\left( (\bH_1-\bI)\bX_2\bbeta_2, \sigma^2(\bH-\bH_1)\right) 
\stackrel{*}{=} \normal\left(\bzero, \sigma^2(\bH-\bH_1)\right),
\end{aligned}
$$
where the equality ($\dag$) follows from the fact that $\bH\bX = \bX$,  the equality ($\ddag$) follows from the fact that $(\bH_1-\bI)\bX_1 = \bzero$, and the equality ($*$) follows since we \text{assume $\mathcalH_0: \bbeta_2=\bzero$}.
By Proposition~\ref{proposition:nested-projection}, we have $(\bH-\bH_1)^\top = \bH-\bH_1$ and $(\bH-\bH_1)^2 = \bH-\bH_1$. 
Thus, it follows that 
$$
\begin{aligned}
\rve - \rve_1 &\sim \normal\left(\bzero, \sigma^2(\bH-\bH_1)^2\right). 
\end{aligned}
$$
The affine transformation of multivariate normal distribution (Lemma~\ref{lemma:affine_mult_gauss}) implies that $\rve-\rve_1$ and $(\bH-\bH_1)\bepsilon$ have the same distribution. Thus, $\normtwo{\rve-\rve_1}^2$ and $\bepsilon^\top (\bH-\bH_1)\bepsilon$ have the same distribution.

By Spectral Theorem~\ref{theorem:spectral_theorem} and Lemma~\ref{proposition:eigenvalues-of-projection} (which states that the only possible eigenvalues of the hat matrix are 0 and 1), we can rewrite by $\normtwo{\rve-\rve_1}^2 \stackrel{d}{=} \bepsilon^\top(\bH-\bH_1)\bepsilon = \bepsilon^\top(\bQ \bLambda\bQ^\top)\bepsilon$, where $\bH-\bH_1=\bQ \bLambda\bQ^\top$ is the spectral decomposition of $\bH-\bH_1$. 
Using Lemma~\ref{lemma:rotat_multi_gauss}, rotations do not affect the distribution of a multivariate normal vector, we thus have
$$
\boldeta \triangleq \bQ^\top\bepsilon \sim \normal(\bzero, \sigma^2\bI)
\qquad\implies \qquad 
\normtwo{\rve-\rve_1}^2 \stackrel{d}{=}  \boldeta^\top \bLambda \boldeta \sim \sigma^2 \chi^2_{\mathrm{rank(\bH-\bH_1)}}\sim\sigma^2 \chi^2_{(p-q)},
$$
where by Lemma~\ref{lemma:rank-of-symmetric-idempotent}, we have 
$$
\begin{aligned}
\rank(\bH-\bH_1) 
&= \trace(\bH) - \trace(\bH_1)
=\trace(\bX (\bX^\top\bX)^{-1}\bX^\top)- \trace(\bX_1 (\bX_1^\top\bX_1)^{-1}\bX_1^\top)\\
&=\trace( (\bX^\top\bX)^{-1}\bX^\top\bX)- \trace( (\bX_1^\top\bX_1)^{-1}\bX_1^\top\bX_1)
=p-q.
\end{aligned}
$$
This completes the proof.
\end{proof}

\index{$F$-test}
\subsection{$F$-test}

Then, combining $\RSS(\widehat{\bbeta})\sim \sigma^2 \chi^2_{(n-p)}$ in Theorem~\ref{theorem:ss-chisquare}, under the hypothesis $\mathcalH_0: \bbeta_2 = \bzero$, we conclude that 
$$
\rT=\frac{\frac{1}{p-q}\left(\RSS(\widehat{\bbeta}_1) - \RSS(\widehat{\bbeta})\right)  }{\frac{1}{n-p}\RSS(\widehat{\bbeta})} \sim F_{p-q,n-p},
$$
which is the \textit{test statistic for the $F$-test}.

\index{$p$-value}
Suppose we have the data set $(\bx_1, y_1), (\bx_2, y_2), \ldots, (\bx_n, y_n)$, and we observe the value $\rT=t$ based on this particular dataset. 
We can then compute the corresponding $p$-value as:
$$
\widetildep=\Pr(\rT((\bx_1, y_1), (\bx_2, y_2), \ldots, (\bx_n, y_n)) \geq t) = \Pr(F_{p-q,n-p} \geq t),
$$
We reject the null hypothesis $\mathcalH_0$ if $\widetildep<\alpha$, for some small $\alpha$, say $\alpha=0.05$.

\index{$F$-distribution}
\begin{SCfigure}%[h!]
\centering
\includegraphics[width=0.5\textwidth]{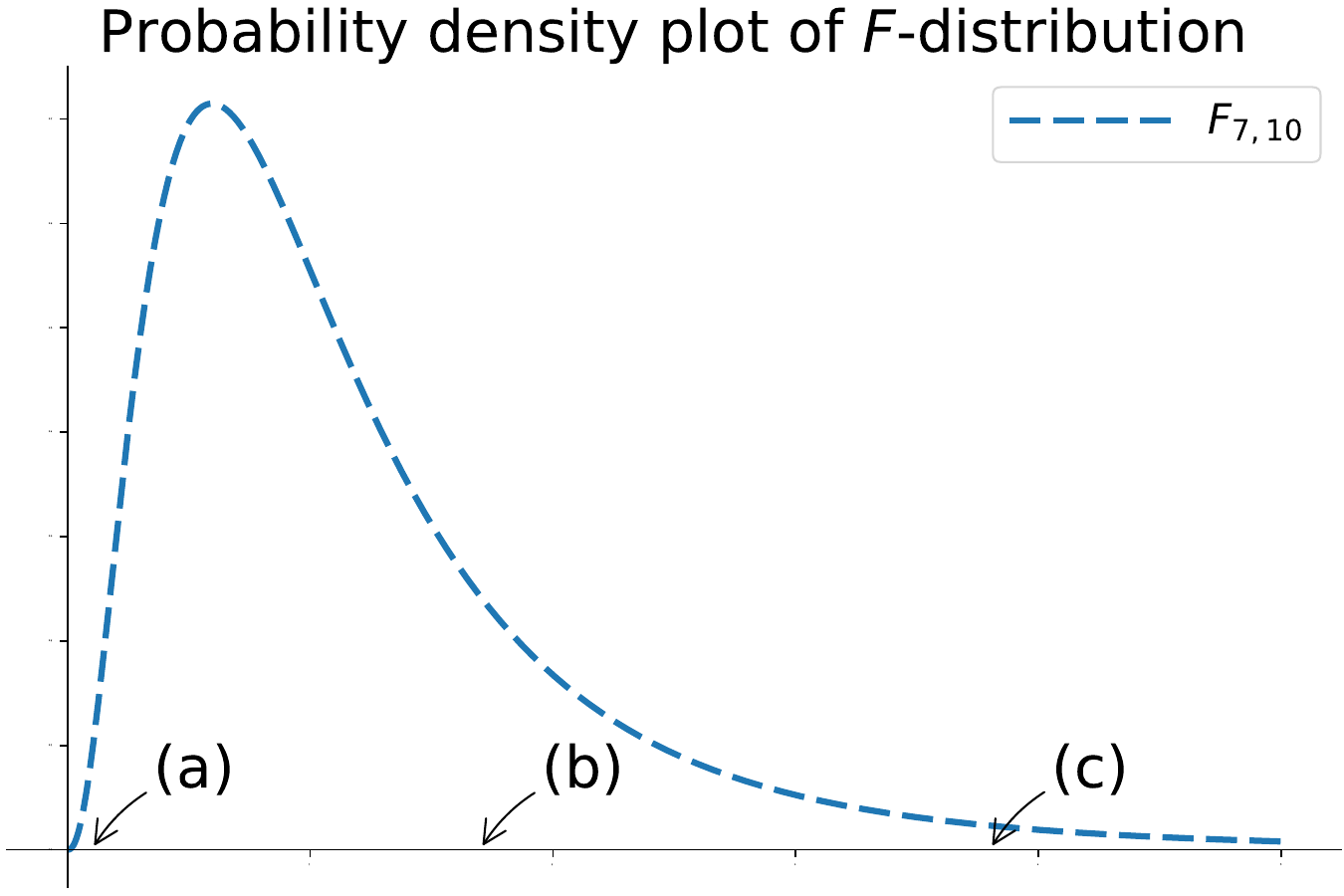}
\caption{An example of $F$-distribution. At point $(b)$, the submodel ($\bbeta_2=\bzero$) seems reasonable and the $p$-value is larger than 0.05 such that we cannot reject the hypothesis. At point $(c)$, the submodel is not reasonable and the $p$-value is smaller than 0.05 such that we reject the hypothesis. At point $(a)$, things are too good and the data may be preprocessed.}
\label{fig:f-dist-example}
\end{SCfigure}

Figure~\ref{fig:f-dist-example} illustrates an example of the $F$-distribution. At point $(b)$, the submodel ($\bbeta_2=\bzero$) seems reasonable and the $p$-value is larger than 0.05 such that we cannot reject the hypothesis. At point $(c)$, the submodel is not reasonable and the $p$-value is smaller than 0.05 such that we reject the hypothesis. At point $(a)$, the model fits too well, raising suspicion that the data may have been preprocessed or otherwise adjusted.

\paragrapharrow{F-test in practice.} 
In computational implementations, it is common to perform the reduced QR decomposition (Theorem~\ref{theorem:qr-decomposition-in-ls}) on the design matrices:
$$
\bX=\bQ\bR
\qquad \text{and}\qquad 
\bX_1 = \bQ_1\bR_1.
$$
Thus, $\RSS(\widehat{\bbeta}) = \rvy^\top (\bI-\bH)\rvy = \rvy^\top\rvy - (\rvy^\top\bQ)(\bQ^\top\rvy)$
and $\RSS(\widehat{\bbeta}_1) - \RSS(\widehat{\bbeta}) = \normtwo{\widehat{\rvy} - \widehat{\rvy}_1}^2 = \rvy^\top(\bH-\bH_1)\rvy=(\rvy^\top\bQ)(\bQ^\top\rvy)-(\rvy^\top\bQ_1)(\bQ_1^\top\rvy)$.
These expressions show that the difference between the residual sums of squares corresponds to the difference between two inner products, which simplifies computation in practice.

%A moment of reflexion would reveal that, the $p$-values is the chance of extreme cases when the hypothesis is true. If the $p$-value is small, then the hypothesis has small chance to be true and the observed data is sufficiently unlikely under the hypothesis, we reject. If not, we fail to reject \footnote{Note that not rejecting is not the equal to accepting the hypothesis.}. 

\index{$F$-test}
\index{Variable selection}
\subsection{Variable Selection Procedure}
In the linear regression model, we are often presented with a large number of potential predictor variables, although many of these may have no meaningful relationship with the response variable $\rvy$. A common method to evaluate the statistical significance of a variable involves setting $\bX_1=\bx_i$, where $\bx_i$ is one column of the design matrix $\bX$, and letting $\bX_2$ denote the remaining columns. This setup produces an associated $p$-value that reflects the variable's contribution to the model.

To identify a more parsimonious set of relevant predictors, we apply a \textit{variable selection procedure}---specifically, \textit{backward elimination}---as detailed in Algorithm~\ref{alg:variable-selection-ftest}. This process iteratively removes variables that contribute the least to the model fit, based on their $p$-values, until all remaining variables meet a specified significance threshold.

\begin{algorithm}[H] 
\caption{Variable Selection Procedure} 
\label{alg:variable-selection-ftest} 
\begin{algorithmic}[1] 
\Require 
Full column rank matrix $\bX \in \real^{n\times p}$;
\For{$i=1$ to $p$} 
\State $\bX_1$ contains a single column of $\bX$, $\bX_2$ contains the remaining columns;
\State \algoalign{Identify  the variable $i_{max}$ with the largest $p$-value  exceeding the cutoff value (e.g., 0.05);}
\State Remove $i_{max}$ from $\bX$;
\EndFor
\State Stop the procedure until all $p$-values are smaller than the cutoff.
\end{algorithmic} 
\end{algorithm}

\subsection{Variable Expansion  Procedure}
The variable selection procedure can be made the other way around---adding variables (i.e., \textit{forward selection}).
Let $\bone$, $\bX_1$, \ldots, $\bX_m$ represent groups of columns in the design matrix $\bX$ (referred to as the ``terms''), such that:
$$
\bX = [\underbrace{\bone}_{n \times 1} \underbrace{\bX_1}_{n \times p_1} \underbrace{\bX_2}_{n \times p_2} \ldots \underbrace{\bX_m}_{n \times p_m}]\in\real^{n\times p}, 
\quad 
\bbeta = [\underbrace{\beta_0}_{1 \times 1} \underbrace{\bbeta_1^\top}_{1 \times p_1} \underbrace{\bbeta_2^\top}_{1 \times p_2} \ldots \underbrace{\bbeta_m^\top}_{1 \times p_m}]^\top\in\real^p.
$$
In this context, the Gauss-Markov model becomes
$$
\rvy = \bX\bbeta + \bepsilon = \bone\beta_0 + \bX_1\bbeta_1 + \cdots + \bX_m\bbeta_m + \bepsilon.
$$
We aim to perform a similar ``F-test investigation", but now on a term-by-term basis. Specifically, we consider the following sequence of nested models:
\begin{itemize}
\item $\rvy = \bone\beta_0 + \bepsilon$.
\item $\rvy = \bone\beta_0 + \bX_1\bbeta_1 + \bepsilon$.
\item $\rvy = \bone\beta_0 + \bX_1\bbeta_1 + \bX_2\bbeta_2 + \bepsilon$.
\item $\vdots$
\item $\rvy = \bone\beta_0 + \bX_1\bbeta_1 + \bX_2\bbeta_2 + \cdots + \bX_m\bbeta_m + \bepsilon$.
\end{itemize}
To proceed, we define
$$
\mathcalX_0 \triangleq \bone
\qquad \text{and}\qquad 
\mathcalX_k \triangleq [\bX_0 , \bX_1, \bX_2 , \ldots , \bX_k], \quad \forall \, k \in \{0, \ldots, m\}.
$$
The corresponding projection matrix (hat matrix), predicted output, and the error vector become:
$$
\bH_k \triangleq\mathcalX_k(\mathcalX_k^\top \mathcalX_k)^{-1} \mathcalX_k^\top, 
\qquad
\widehat{\rvy}_k \triangleq \bH_k \rvy,
\qquad 
\rve_k = \rvy - \widehat{\rvy}_k, \qquad \forall\, k \in \{0, \ldots, m\}.
$$
Note that $\widehat{\rvy}_0 = \overline{\ry}\bone$, where  $\overline{\ry} = \frac{1}{n}\sum_{i=1}^{n} \ry_i$ represents the sample mean of the observed response values.
Similar to the argument in \eqref{equation:var_sec_two}, it follows that 
\begin{align*}
\underbrace{\normtwo{\rvy - \widehat{\rvy}_0}^2}_{\normtwo{\rve_0}^2} &= \underbrace{\normtwo{\rvy - \widehat{\rvy}_m}^2}_{\normtwo{\rve_m}^2} + \underbrace{\normtwo{\widehat{\rvy}_m - \widehat{\rvy}_{m-1}}^2}_{\normtwo{\rve_m - \rve_{m-1}}^2} + \cdots + \underbrace{\normtwo{\widehat{\rvy}_1 - \widehat{\rvy}_0}^2}_{\normtwo{\rve_1 - \rve_0}^2} 
= \underbrace{\normtwo{\rve_m}^2}_{\RSS_m} + \sum_{k=0}^{m-1} \underbrace{\normtwo{\rve_{k+1} - \rve_k}^2}_{\RSS_k - \RSS_{k+1}},
\end{align*}
where $\RSS_k$  denotes the residual sum of squares for the fitted model $\widehat{\rvy}_k$, with $\nu_k=n-1-\sum_{i}^{k}p_i$ degrees of freedom; see Remark~\ref{remark:df_of_a}.
According to Theorem~\ref{theorem:anova-model-selection}, these components can be interpreted as:
\begin{itemize}
\item $\RSS_k - \RSS_{k+1}$ is the reduction in residual sum of squares caused by adding the term $\bX_{k+1}$ to a model that already includes $\mathcalX_0,\bX_1, \ldots, \bX_k$.
\item $\RSS_m$ and $\{\RSS_k - \RSS_{k+1}\}_{k=0}^{m-1}$ are all mutually independent.
\item Since $\nu_k=n-1-\sum_{i}^{k}p_i$, we have $\nu_0 \geq \nu_1 \geq \nu_2 \geq \cdots \geq \nu_m$; and $\nu_{k+1} = \nu_k$ if $\bX_{k+1} \in \cspace(\mathcalX_k)$.
\end{itemize}

\begin{table}[H]
\centering
\resizebox{1.\textwidth}{!}{%}
\begin{tabular}{l c c | l c c}
\hline
Terms & df & Residual $\RSS$ & Terms added & df & Reduction in $\RSS$ \\
\hline
$\bone$ & $n-1$ & $\RSS_0$ & & &  \\
$\bone, \bX_1$ & $\nu_1$ & $\RSS_1$ & $\bX_1$ & $n-1-\nu_1$ & $\RSS_0 - \RSS_1$  \\
$\bone, \bX_1, \bX_2$ & $\nu_2$ & $\RSS_2$ & $\bX_2$ & $\nu_1 - \nu_2$ & $\RSS_1 - \RSS_2$ \\
$\vdots$ & $\vdots$ & $\vdots$ & $\vdots$ & $\vdots$ & $\vdots$  \\
$\bone, \bX_1, \ldots, \bX_m$ & $\nu_m$ & $\RSS_m$ & $\bX_m$ & $\nu_{m-1} - \nu_m$ & $\RSS_{m-1} - \RSS_m$  \\
\hline
\end{tabular}
}
\captionof{table}{ANOVA table for variable expansion procedure.}
\label{table:avo_var_expansion}
\end{table}

The $F$-statistic for testing the significance of the reduction in $\RSS$ when $\bX_k$ is added to the model containing the terms $\mathcalX_0, \bX_1, \ldots, \bX_k$ is given by:
$$
\rT_k = \frac{(\RSS_{k-1} - \RSS_k) / (\nu_{k-1} - \nu_k)}{\RSS_m / \nu_m},
$$
and $\rT_k \sim F_{\nu_{k-1}-\nu_k, \nu_m}$ under the null hypothesis $\mathcalH_0 : \bbeta_k = \bzero$.
Large values of $\rT_k$ relative to the null distribution are evidence against $\mathcalH_0$; see Figure~\ref{fig:f-dist-example}.

Using these results, the \textit{variable expansion procedure} considers adding each term in the model sequentially, using F-test to evaluate the significance of each addition in the reduction of the RSS; see Table~\ref{table:avo_var_expansion}.

\index{Model selection}
\section{Model Selection}

We previously introduced several metrics for evaluating a model's performance.
Unlike model evaluation, \textit{model selection} refers to the process of choosing the best model from a set of candidate statistical models based on a given dataset. Model selection can be applied to different types of models---such as logistic regression, neural networks, Gaussian process, etc.---as well as to models of the same type with varying hyper-parameters (e.g., ordinary least squares models with different numbers of predictor variables).

Before exploring various model selection methods and their appropriate use cases, it's important to understand the distinction between model selection and model evaluation:
\begin{itemize}
\item Model evaluation focuses on the model's performance during the training or fitting phase. It evaluates the fitting error of each candidate model to determine which one fits the training data best.
\item Model selection, on the other hand, aims to estimate the generalization error of the selected model---that is, how well the model performs on unseen data.
\end{itemize}

The key reason lies in the risk of \textit{overfitting}. A model may perform exceptionally well on the training data---for instance, a saturated model that perfectly fits every training sample---but fail to generalize to new data. Therefore, a good model should not only fit the training data well but also maintain strong performance on unseen data. Before deploying any model, we must ensure its performance remains stable when exposed to new inputs.

While training a model is relatively straightforward, selecting an appropriate model is often much more challenging. First, we need to move beyond the idea of a single ``best" model. Due to noise in the data, incomplete samples, and the limitations inherent to different modeling techniques, all models carry some degree of prediction error. As a result, the notion of a perfect or universally best model is not practical. Instead, we should aim to find a ``good enough" model.
A famous quote by \citet{box1987empirical} captures this ideal well: 
$$
\textit{``All models are wrong, but some are useful."}
$$

Different application scenarios come with different priorities when  choosing a final model. These might include: 
\begin{itemize}
\item \textit{Interpretability.} The ease with which stakeholders can understand and trust the model.
\item \textit{Model complexity.} Simpler models may be preferred in production environments.
\item \textit{Maintainability.} How easy it is to update or retrain the model over time.
\item \textit{Computational efficiency.} Some applications require fast inference or low memory usage.
\end{itemize}
In some contexts, a slightly less accurate but highly interpretable model may be favored. In others, raw performance may be prioritized at the expense of computational cost. Thus, what constitutes a ``good enough" model depends heavily on the specific goals and constraints of the project.

Roughly speaking, there are typically three main strategies used for model selection:
\begin{itemize}
\item \textit{Train, validation, and test sets.} Use a large dataset split into training, validation, and test sets to select the best-performing model.
\item \textit{Probabilistic measures.} Select models using a combination of sample error and model complexity.
\item \textit{Resampling methods.} Estimate out-of-sample error through repeated sampling techniques like cross-validation.
\end{itemize}
In an ideal scenario where sufficient data is available, the most reliable method is to split the data into three parts:
\begin{itemize}
\item \textit{Training set.} Fit the candidate models.
\item \textit{Validation set.} Tune hyper-parameters and select among competing models.
\item \textit{Test set.} Evaluate the final model's generalization performance using metrics such as accuracy or mean squared error.
\end{itemize}
However, this approach requires a substantial amount of data, which is often unavailable. In practice, especially when data is limited, probabilistic measures and resampling methods become more widely used.
Resampling methods are simple to implement even with small datasets. Since these two approaches are relatively intuitive, they don't require detailed explanation here. The rest of this section will focus on introducing commonly used probabilistic measure methods for model selection.

\subsection*{Probabilistic Measures}
Probabilistic measures evaluate candidate models based on two key factors: their performance on the training data and their model complexity. The concept of model complexity plays a central role in developing metrics that guide effective model selection.

It is well known that training error tends to be overly optimistic, making it an unreliable basis for model selection. This optimism can be addressed by penalizing model complexity, especially in algorithm-specific ways---often applicable to linear models. Several ``information criteria" have been proposed over time to address this issue by introducing penalty terms that counteract the bias toward more complex models.

According to \textit{Occam's Razor}, when multiple models offer similar predictive or explanatory power, the simplest one is often the best choice. Models with fewer parameters are generally less complex and tend to generalize better on average. Two widely used probabilistic model selection criteria include:
\begin{itemize}
\item  Akaike information criterion (AIC).
\item Bayesian information criterion (BIC).
\end{itemize}

When working with simple linear models---such as linear regression or logistic regression---probabilistic measures are particularly appropriate. In such cases, quantities like sample variance (used in computing model complexity) are well-defined and straightforward to calculate.

For example, both AIC and BIC penalize the number of model parameters while rewarding goodness of fit on the training set. Therefore, the best model is the one with the lowest AIC or BIC value. However, BIC imposes a stronger penalty on model complexity than AIC, which means it tends to favor simpler models---even if they perform slightly worse in terms of fit. Although these criteria allow model selection without using a validation set, they were originally derived under assumptions that apply strictly to parametric  linear models. Nonetheless, they are also commonly applied to broader settings, such as generalized linear models (e.g., Poisson or Gamma regression models; see Chapter~\ref{chapter:glm}).

\index{Akaike information criterion}
\index{AIC}
\subsection{AIC}
The \textit{Akaike information criterion (AIC)} is named after its creator, the statistician \textit{Hirotugu Akaike}.
Today, it serves as a foundational tool and is widely used in statistical inference.
AIC estimates the out-of-sample prediction error and measures the relative quality of statistical models fitted to a given dataset. It evaluates each model in comparison to others, making it a valuable method for model selection.

AIC is grounded in information theory. When we use statistical models to approximate the true data-generating process, they are rarely exact---some information is inevitably lost. AIC estimates this relative information loss, with better models losing less information. In doing so, AIC balances two key components: goodness of fit and model complexity. This balance helps prevent both overfitting (too complex models) and \textit{underfitting} (too simple models).

AIC can be used to compare both nested and non-nested models. It quantifies the information loss of each candidate model, aiming to identify the one that minimizes this loss. The formula for AIC is:
\begin{equation}
\AIC = 2p - 2\ell,
\end{equation}
where $\ell(\cdot)$ denotes the log-likelihood or measure of model fit, $p$ represents the model's complexity (the number of parameters in the model). 
It is important to note that the absolute value of AIC has no meaning; rather, it is the difference between AIC values across models that is informative. When comparing two models:
\begin{itemize}
\item If their fitting abilities (i.e., likelihoods) differ significantly, AIC is primarily influenced by the likelihood term.
\item If the fits are similar, AIC becomes more sensitive to the number of parameters $p$, which acts as a penalty for model complexity.
\end{itemize}
This penalty discourages overfitting by favoring simpler models unless a more complex model provides a substantially better fit.

AIC is particularly useful when comparing generalized linear models (GLMs, see Chapter~\ref{chapter:glm}) that share the same link and variance functions but differ in the set of predictor variables. When models are nested, the penalty term reflects the precision needed to remove unnecessary predictors from the model.

%Note that the absolute value of AIC is meaningless; the relative size between models is significant. When comparing two models with significantly different fitting abilities (likelihood), AIC is more influenced by the likelihood values; when the fitting abilities (likelihood) of the two models are similar, AIC is more influenced by the number of model parameters $p$.
%
%The parameter count $p$ is a penalty for a larger list of parameter variables. AIC is particularly suitable for comparing models with the same link function and variance function but different parameter variable lists in generalized linear models (GLMs, see Chapter~\ref{chapter:glm}). When models are nested, we consider the penalty term as the precision required to eliminate candidate predictive variables from the model.

%We need to pay attention to how to calculate AIC. The above definition includes the model log-likelihood. In GLM, the parameter estimation process is usually not based on likelihood but on deviance. Deviance cannot be used to calculate AIC because the calculation process does not include the normalization term $c(y_i, \rho)$, and different distributions have different normalization terms in GLM.

%Later, several variants of AIC were developed. Here, we provide two alternative methods. The first is the bias AIC proposed by \citet{sugiura1978further, hurvich1989regression}. The second is the $\AIC_{\text{hq}}$ described by \citet{hannan1979determination}. The formulas for these versions are:

Several alternative versions of AIC have been proposed to improve its performance under certain conditions. Two notable variants include:
\begin{itemize}
\item \textit{Corrected AIC ($\AIC_{\text{c}}$).} Introduced by \citet{sugiura1978further} and later refined by \citet{hurvich1989regression}, this version adjusts for small sample sizes:
$$
\AIC_{\text{c}} = 2 \frac{p(p + 1)}{n - p - 1} + 2p - 2\ell,
$$
where, again,  $p$ is number of parameters in the model, and $n$ is number of observed samples.
\item \textit{Hannan-Quinn criterion ($\AIC_{\text{hq}}$).} Proposed by \citet{hannan1979determination}, this variant uses a slightly different penalty function:
$$
\AIC_{\text{hq}} = 2p \ln\{\ln(n)\} - 2\ell.
$$
\end{itemize}

\index{Bayesian information criterion}
\index{Schwarz criterion}
\index{BIC}
\subsection{BIC}\label{section:bic}

In statistics, the  \textit{Bayesian information criterion (BIC, a.k.a.  Schwarz criterion or Schwarz information criterion)} is a standard used to select models from a finite set of models. The model with the lowest BIC is preferred. 
It is partly based on the likelihood function and is closely related to AIC:
\begin{equation}
\BIC_\ell = p \ln(n) - 2\ell,
\end{equation}
where $p$ is the number of parameters in the model, $n$ is the sample size, and $\ell$ is the log-likelihood.
Unlike AIC, BIC includes a penalty term that becomes more severe as the sample size increases, making it more sensitive to model complexity as more data becomes available. This feature enhances its ability to identify meaningful patterns in the data.

To understand the theoretical foundation of BIC, we adopt the notation introduced in the Bayesian estimation section (see Section~\ref{section:bayes_esti}) and assume the hyper-parameter $\balpha$ on the prior; see \eqref{equation:baye_hyper}.
We rewrite the Laplace approximation of marginal likelihood with its complexity under data size in \eqref{equation:laplace_approx} as follows:
$$
\ln p(\mathcalX \mid \balpha)_{\text{Lap}}  = 
\underbrace{\ln p(\mathcalX \mid \widehat{\btheta} )}_{\mathcalO(n)} + \underbrace{\ln p(\widehat{\btheta} \mid \balpha)}_{\mathcalO(1)} + 
\underbrace{\frac{p}{2} \ln (2\pi) }_{\mathcalO(1)}
- 
\underbrace{\frac{1}{2}\ln \abs{\nabla^2 \ell(\widehat{\btheta} )}}_{\mathcalO(p\ln n)},
$$
where $\widehat{\btheta}$ denotes the selected model parameter.
The BIC score considers only the terms growing with data size $n$, and as the entries of the Hessian scale linearly with $n$, we approximate the marginal likelihood as \citep{schwarz1978estimating}:
$$
\begin{aligned}
\ln p(\mathcalX \mid \balpha)_{\text{Lap}} 
&\approx 
\ln p(\mathcalX \mid \widehat{\btheta} ) -\frac{1}{2}\abs{\nabla^2 \ell(\widehat{\btheta} )}
\overset{n\rightarrow \infty}{\approx}
\ln p(\mathcalX \mid \widehat{\btheta} ) - \lim_{n\rightarrow \infty }\frac{1}{2}\abs{\nabla^2 \ell(\widehat{\btheta} )}\\
&= \ln p(\mathcalX \mid \widehat{\btheta} ) -\frac{1}{2}\abs{n \bH_0} = \ln p(\mathcalX \mid \widehat{\btheta} ) - \frac{p}{2}\ln n - \underbrace{\frac{1}{2}\ln \abs{\bH_0}}_{\mathcalO(1)}. 
\end{aligned}
$$
Thus, the BIC score can be approximated as:
$$
\BIC_\ell
\approx
-2\ln p(\mathcalX \mid \balpha)_{\text{Lap}} 
\approx
p\ln n -2 \ln p(\mathcalX \mid \widehat{\btheta} ).
$$

The BIC offers several notable advantages.
Notably, BIC includes a penalty term that increases with the number of parameters ($p$) in the model (same as the AIC score), which helps to prevent overfitting by discouraging overly complex models.
BIC is straightforward to compute and interpret. It does not require any additional assumptions beyond those inherent in the models being compared.
Under certain regularity conditions, BIC is a consistent model selection criterion, meaning that as the sample size increases, it will select the true model with probability approaching one, provided the true model is included in the set of candidate models.
Unlike full Bayesian model selection, BIC does not require specifying prior distributions for the parameters, making it more accessible to users who are not familiar with Bayesian methods.
However, from a Bayesian perspective, the last feature might be seen as a drawback.

On the other hand, the BIC score does not take into account the local geometry of the parameter space. As a result, it is invariant to reparameterizations of the model. This invariance is desirable  since the BIC's insensitivity to how the parameters are expressed ensures that the criterion yields consistent results regardless of the parameterization chosen, aligning with the expectations of a rigorous Bayesian analysis where the posterior distribution should be invariant to reparameterization \citep{hoff2009first}. This property enhances the reliability and interpretability of the BIC when comparing different models, as it avoids bias introduced by arbitrary choices in parameter definitions \citep{beal2003variational}.

\index{Occam's razor}
\index{Occam factor}
\subsection{Occam's Razor and Occam Factor}\label{section:occam_razor}
In the context of BIC scores, we associate the complexity of a model with the number of parameters it has so as to prevent overfitting by discouraging overly complex models.
However, it is easy to come up with a model with numerous  parameters that can only represent a narrow variety of data sets; or conversely, to develop a model that can encompass a vast array of data sets using merely a single parameter.
In this scenario, it is wise to discard models that are too complex or too simple using marginal likelihood:
$$
p(\mathcalX \mid \balpha)= \int p(\btheta \mid \balpha)p(\mathcalX \mid \btheta) d\btheta,
$$
where we integrate out the parameters $\btheta$ and penalize models with more degrees of freedom, as such models have the capacity to fit a wide range of data sets a priori.
This property of Bayesian integration is known as \textit{Occam's razor}, which is the principle that states a preference for  simpler models for the data over complex ones \citep{mackay1995probable, beal2003variational}.

\begin{figure}[h!]
\centering
\includegraphics[width=0.65\textwidth]{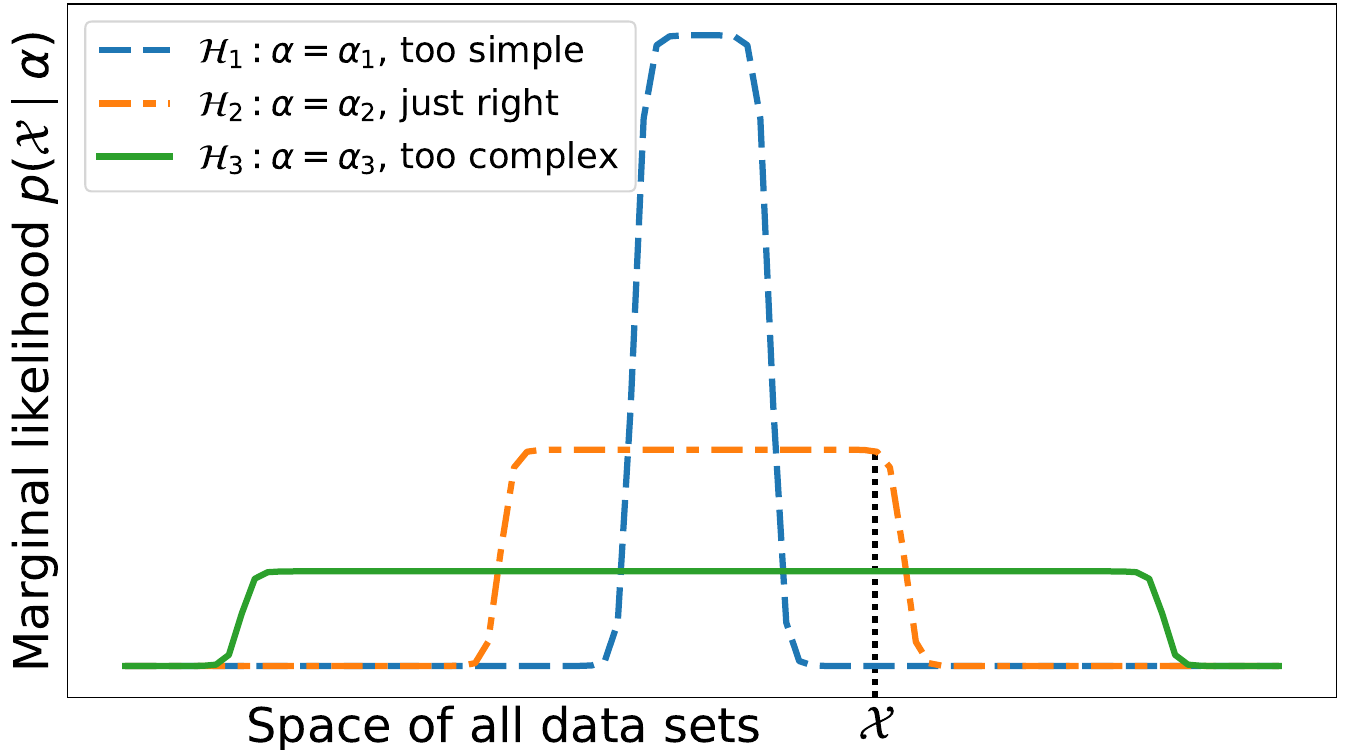}
\caption{Bayesian inference embodies Occam's razor. This figure provides the fundamental intuition for why more complex models tend to be less probable. The horizontal axis represents the space of all possible data sets, $\mathcalX$. 
According to Bayes' theorem, models are favored in proportion to how well they predicted the observed data. These predictions are represented by a marginal probability distribution over $\mathcalX$. 
A simple model  makes only a limited range of predictions; while a more powerful model is capable of predicting a greater variety of data sets.}
\label{fig:occam_razor}
\end{figure}

Occam's razor is shown in   Figure~\ref{fig:occam_razor}.
Since the probability of different data sets integrate to one over the marginal likelihood $p(\mathcalX \mid \balpha)$. If the the model is overly complex such that it can model a vast variety of data sets, the probability value for each data set can be reduced (the ``too complex" case in the figure with hypothesis \{$\mathcalH_3: \balpha=\balpha_3$\}). 
While the model is too simple, it might not cover the observed data set, rendering a small marginal probability (the ``too simple" case in the figure with hypothesis \{$\mathcalH_1: \balpha=\balpha_1$\}).

In Figure~\ref{fig:occam_razor}, the model hypothesis $\mathcalH_1$ is not compatible with the observed data set $\mathcalX$.
However, in the case where the data are compatible with both theories $\mathcalH_2$ and $\mathcalH_3$, the simpler model $\mathcalH_2$ will turn out to be more probable than the more complex model $\mathcalH_3$, without us having to express any subjective bias against complex models. Our subjective prior should simply assign equal probabilities to the possibilities of simplicity and complexity.
Therefore, given a data set $\mathcalX$, it is possible to discard both models that are too complex and those that are too simple, based on their marginal likelihood.

\begin{figure}[h!]
\centering
\includegraphics[width=0.65\textwidth]{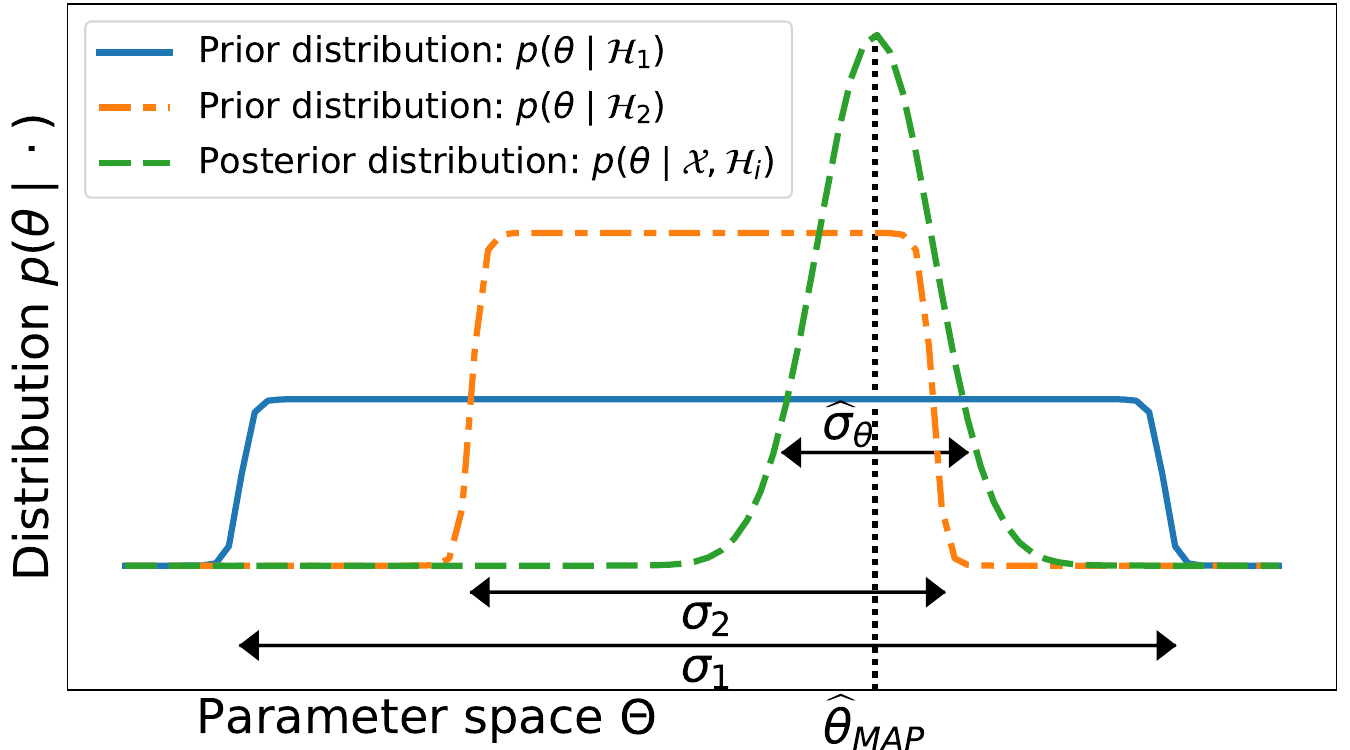}
\caption{Occam factor. The prior distribution $p(\btheta \mid \mathcalH_1)$ for the
	parameter has width $\sigma_{1}$, and the prior distribution $p(\btheta \mid \mathcalH_2)$ for the
	parameter has width $\sigma_{2}$. The posterior distribution  has a single peak at $\widehat{\btheta}_{\text{MAP}}$ with width $\widehat{\sigma}_{\btheta}$. }
\label{fig:occam_factor}
\end{figure}

As mentioned previously (see Section~\ref{section:bayes_esti}), the marginal likelihood or evidence is usually intractable or impossible to compute. Bayesian Occam's razor provides a way to approximate the marginal likelihood \citep{mackay1995probable}. As a recap, the marginal likelihood under a hypothesis $\{\mathcalH_1: \balpha=\balpha_1\}$ is 
$$
p(\mathcalX \mid \mathcalH_1)= \int p(\btheta \mid \mathcalH_1)p(\mathcalX \mid \btheta) d\btheta.
$$
For many problems, it is not uncommon that the posterior distribution $p(\btheta \mid \mathcalX, \mathcalH_1)= \frac{p(\mathcalX \mid \btheta)p(\btheta \mid \mathcalH_1)}{\text{marginal likelihood}}$ has a strong peak at the most probable parameter $\widehat{\btheta}_{\text{MAP}}$, i.e., the MAP estimate (see Figure~\ref{fig:occam_factor}).
Therefore, the marginal likelihood can be approximated by the height of the peak of the integrand $p(\btheta \mid \mathcalH_1)p(\mathcalX \mid \btheta)$ times its width, denoted by $\widehat{\sigma}_{\btheta}$ (see Figure~\ref{fig:occam_factor}):
$$
\underbrace{p(\mathcalX \mid \mathcalH_1)}_{\text{marginal likelihood}} \approx 
\underbrace{p(\mathcalX \mid \widehat{\btheta}_{\text{MAP}})}_{\text{MAP fit likelihood}} 
\underbrace{p(\widehat{\btheta}_{\text{MAP}} \mid \mathcalH_1) \cdot\widehat{\sigma}_{\btheta}}_{\text{Occam factor}},
$$
where $(\widehat{\btheta}_{\text{MAP}} \mid \mathcalH_1) \cdot \widehat{\sigma}_{\btheta}$ is defined as the \textit{Occam factor}~\footnote{When the posterior is approximated by a Gaussian, then the width is obtained by the determinant of the covariance matrix: $\widehat{\sigma}_{\btheta}=\det^{-1/2} \left(-\frac{1}{2\pi} \nabla^2 \ln p(\widehat{\btheta}_{\text{MAP}} \mid \mathcalX, \mathcalH_1)\right)$. See \citet{mackay1995probable} for more details.}.
The Occam factor is a value smaller than one  if $\widehat{\sigma}_{\btheta}< \sigma_{1}$, where the latter is the width of the prior distribution $p(\btheta \mid \mathcalH_1)$ (see Figure~\ref{fig:occam_factor}), and acts as a regularization that penalizes the parameter $\btheta$.

The width of the  posterior distribution signifies the uncertainty in parameter $\btheta$; while the width of the prior distribution represents the range of  values that were possible a priori.
Suppose the prior $p(\btheta \mid \mathcalH_1)$ is uniform. Then $p(\btheta \mid \mathcalH_1)=\frac{1}{\sigma_{1}}$, and the Occam factor is 
$$
\mathcalO_1=\frac{\widehat{\sigma}_{\btheta}}{\sigma_{1}},
$$
which measures the magnitude by which the hypothesis space collapses when the data arrive.
The model $\mathcalH_1$ can be viewed as consisting of a certain number of exclusive submodels, of which only one remains viable upon receiving the data. The Occam factor is the inverse of this number. The logarithm of the Occam factor measures the amount of information we gain about the model's parameters when the data become available \citep{mackay1995probable}.

Assume further  there is a hypothesis $\{\mathcalH_2:\balpha=\balpha_2\}$ with a smaller width $\sigma_2<\sigma_1$.
And assume the posterior distribution under $\mathcalH_2$ and $\mathcalH_1$ are the same: $p(\btheta \mid \mathcalX, \mathcalH_i)$ with the same width $\widehat{\sigma}_{\btheta}$ (this is a strong assumption for ease of evaluation; see Figure~\ref{fig:occam_factor}). 
The corresponding Occam factors has the following relationship:
$$
\mathcalO_1=\frac{\widehat{\sigma}_{\btheta}}{\sigma_{1}} < \mathcalO_2=\frac{\widehat{\sigma}_{\btheta}}{\sigma_{2}}.
$$
Apparently, model $\mathcalH_2$ is a stronger prior (more complex in a sense) than model $\mathcalH_1$ since the former imposes more a priori information in the assumption.
Therefore, the magnitude of the Occam factor serves as a measure of the model's complexity. This depends not only on the number of parameters within the model but also on the prior probability distribution the model assigns to those parameters.

\begin{problemset}
\item \label{prob:cook_b_bi} Prove \eqref{equation:cook_b_bi}. \textit{Hint: Denoting $\bX=\begin{bmatrixfoot}
\bX_1\\
\bx_i^\top\\
\bX_2 
\end{bmatrixfoot}$ and $\widetildebX =\begin{bmatrixfoot}
\bX_1\\
\bX_2 
\end{bmatrixfoot}$. Then $(\widetildebX^\top\widetildebX)^{-1} = (\bX^\top\bX)^{-1} + \frac{(\bX^\top\bX)^{-1}\bx_i\bx_i^\top (\bX^\top\bX)^{-1}}{1-h_{ii}}$ using the Sherman-Morrison formula \eqref{equation:she_morr_for}.}
\item Under the discussed model, suppose the probability density function for an $F$-test follows $F_{7,10}$, and with a given  threshold of $\alpha=0.05$, determine the critical value for the test statistic that rejects the hypothesis.
\end{problemset}

\newpage
\chapter{Large-Scale Least Squares Approximations}\label{chapter:large_scale}
\begingroup
\hypersetup{
	linkcolor=structurecolor,
	linktoc=page,  % page: only the page will be colored; section, all, none etc
}
\minitoc \newpage
\endgroup

\lettrine{\color{caligraphcolor}R}
Randomized methods help address large-scale optimization problems by offering faster and more efficient solutions to complex linear algebra tasks, enhancing the performance of iterative solvers through better preconditioning and enable scalable solutions for handling massive datasets.
These methods can significantly aid in solving large-scale optimization problems through several key mechanisms:
\begin{itemize}
\item \textit{Random sampling and projection.} Randomized algorithms often use randomness to perform sampling or projection operations on matrices. This involves selecting a small number of columns, rows, or elements from the matrix in a strategic manner to highlight important structural features. Alternatively, data can be projected into a lower-dimensional space while preserving key characteristics of the original dataset. The goal is to create a ``sketch" of the original data---similar in essential properties but simpler and faster to process.

\item \textit{Improved computational efficiency.} By using these sketches, randomized methods can significantly reduce the computational burden associated with solving large-scale linear algebra problems, such as matrix multiplication, least-squares regression, and low-rank matrix approximation. As a result, they make it feasible to solve problems that would otherwise be computationally infeasible due to their size.

\item \textit{Scalability.} Randomized methods are designed to scale well with the size of the input data, making them ideal for applications in machine learning, statistical data analysis, and other fields where datasets continue to grow. These techniques allow for the efficient processing of large datasets on both single machines and distributed systems without compromising accuracy.

\end{itemize} 

\noindent
In this chapter, we will primarily introduce randomized algorithms for solving large-scale least-squares problems.

\section{Sketched Least Squares}

\subsection{General Ideas and Subspace Embedding}
This subsection will not tend to be rigorous, which is the main goal of this book.
We will collect some important theorems in this section without  proof. 
This subsection is standalone, which aims to introduce the fundamental notions of subspace embedding, sketching technique, and its complexity in matrix multiplication. 
For further details, readers may refer to the references cited throughout the text.

\subsection*{Sketching Technique}
Before exploring large-scale least squares approximations, we briefly review the sketching technique used in approximate matrix multiplication.
To motivate this approach, we begin with a fundamental problem: approximating the product of two matrices. 
Suppose that we are given an  $n\times p$ matrix $\bX$ and a $p\times m$ matrix $\bY$, and our goal is to compute their product  $\bX\bY$:
\begin{itemize}
\item \textit{Three-loop perspective.} The most straightforward method for computing $\bX\bY$ is the classic three-loop algorithm. 
In this approach, each entry of the resulting matrix is viewed as the inner product between a row of $\bX$ and a column of $\bY$:
\begin{equation}
(\bX\bY)_{ij} = \innerproduct{\bx^{(i)}, \by_j}, \quad \forall\, i, j,
\end{equation}
where $\bx^{(i)}$ denotes the $i$-th row of $\bX$ (treated  as a column vector), and $\by_j$ represents the $j$-th column of $\bY$.

\item \textit{Column-row perspective.} An alternative way to understand the product $\bX\bY$ is as the sum of $p$ outer products, each of which is formed by a column of $\bX$ and the corresponding row of $\bY$:
\begin{equation}
\bX\bY = \sum_{i=1}^{p} \bx_i \by^{(i)\top}.
\end{equation}
\end{itemize}
From the column-row perspective, we can attempt to construct a simplified ``sketch" of the columns of $\bX$ and the rows of $\bY$.
These sketches are represented as matrices $\bC$ and $\bR$, respectively, and we approximate $\bX\bY$ using the product $\bC\bR$.

Such a sketch can be constructed using a $p\times k$ (with $k<p$) matrix $\bS$ (since we consider only linear sketches) such that  $\bC = \bX \bS\in\real^{n\times k}$ and $\bR = \bS^\top \bY\in\real^{k\times m}$.
Here, $\bX\bS$ performs a right-sketch on the columns of $\bX$, while $\bS^\top\bY$ applies a left-sketch to the rows of  $\bY$.
Our objective is to approximate $\bX\bY $ using6 $\bX\bY \approx \bC\bR = \bX \bS\bS^\top \bY$, where the quality of this approximation is typically evaluated by bounding the norm of the error matrix, i.e., by providing an upper bound for $\norm{\bX\bY - \bX \bS\bS^\top \bY}_\xi$, where $\xi$ represents some matrix norm such as the spectral norm, Frobenius norm, or trace norm. 

Of course, described this way, {the sketching matrix $ \bS $ can be anything---deterministic or randomized, efficient or intractable to compute, etc.}
However, it turns out that when $\bS$ is randomized---based on techniques like random sampling or random projections---we often achieve better performance compared to deterministic approaches.

The randomized sketches  fall into one of two categories:
\begin{itemize}
\item \textit{Random sampling sketches.} In this case, each column of the matrix $\bS$ contains exactly one nonzero entry, indicating which (rescaled) column of $\bX$ is selected.
\item  \textit{Random projection sketches.} 
That is, the matrix $ \bS $ is dense or nearly dense, consisting of i.i.d. random variables---often drawn from a distribution such as the standard Gaussian.
\end{itemize}
Random projections are considered \textit{data-agnostic}, meaning they can be constructed without observing the input data. In contrast, random sampling usually requires identifying and extracting important structural features from the data.

Basic  versions of both methods often perform reasonably well---they typically produce acceptable approximations---but they are not always faster than solving the original problem directly. Therefore, more sophisticated variants have been developed. For example, the sketching matrix $\bS$ might combine a Hadamard-like transformation with uniform sampling or a sparse projection matrix. In such cases, the quality of the projection or sampling remains close to that of the basic method, but the computation is significantly faster.

However, the main goal of this chapter is to illustrate the core ideas using Gaussian sketching. For details on other techniques, readers are encouraged to consult the references provided throughout the text.

\index{Subspace embedding}
\subsection*{Subspace Embedding}

We provide the basic concept  of an $\ell_2$-subspace embedding for the column space of an $n \times p$ matrix $\bX$. 
As we will see, this is a powerful tool for solving least squares regression problems. Throughout this section,  for nonnegative real numbers $x$ and $y$, we write $x \in (1 \pm \varepsilon)y$ to mean $x \in [(1 - \varepsilon)y, (1 + \varepsilon)y]$.

\begin{definition}[$(1 \pm \varepsilon)$ $\ell_2$-subspace embedding]\label{defi:subspace-embedding}
Let $\bX\in\real^{n\times p}$. Then, a \textit{$(1 \pm \varepsilon)$ $\ell_2$-subspace embedding} for the column space of   $\bX$ is a matrix $\bS$ satisfies
$$
\normtwo{\bS \bX \bbeta}^2 \in (1 \pm \varepsilon) \normtwo{\bX \bbeta}^2 ,\quad \forall\, \bbeta \in \real^p.
$$
\end{definition}
In what follows, we will often use shorthand and say that $\bS$ is an $\ell_2$-subspace embedding for $\bX$ itself, even though the definition depends only on the column space of $\bX$, not on the specific basis used to represent it.

Observe that if $\bS$ is a $(1 \pm \varepsilon)$ $\ell_2$-subspace embedding for $\bX$, then it is also a $(1 \pm \varepsilon)$ $\ell_2$-subspace embedding for $\bQ\in\real^{n\times r}$, where $\bQ$ is an orthonormal basis for the column space of $\bX$ and $\rank(\bX)=r$. 
This is because the sets $\{ \bX \bbeta \mid \bbeta \in \real^p \}$ and $\{ \bQ \balpha \mid \balpha \in \real^r \}$ are identical. 
Therefore, without loss of generality, we may assume that $\bX$ has orthonormal columns.
Under this assumption, the condition in Definition~\ref{defi:subspace-embedding} becomes:
\begin{equation}
\normtwo{ \bS \bQ \balpha}^2 = (1 \pm \varepsilon) \normtwo{\bQ \balpha}^2 = (1 \pm \varepsilon) \normtwo{\balpha}^2,
\end{equation}
where the last equality follows since $\bQ$ has orthonormal columns. 
If this requirement holds for all unit vectors $\balpha$,  then by linearity of $\bS$, it holds for all $\balpha\in\real^r$ (e.g., by scaling).
Thus, the requirement can be further simplified to:
\begin{equation}
\normtwo{\bI_r - \bQ^\top \bS^\top \bS \bQ} \leq \varepsilon.
\end{equation}

There are several goals in designing subspace embeddings. Two primary ones are:
\begin{itemize}
\item To find a matrix $\bS$ with as few rows as possible.
\item To ensure that the product $\bS\bX$ can be computed efficiently, since this is often a computational bottleneck in applications.
\end{itemize}

There are many ways to construct $\ell_2$-subspace embeddings, each offering different trade-offs between efficiency, accuracy, and other constraints. One particularly useful type is the oblivious $\ell_2$-subspace embedding.

\begin{definition}[$(\varepsilon, \delta)$ oblivious $\ell_2$-subspace embedding]\label{def:oblivious-subspace-embedding}
Let $P$ be a distribution on $r \times n$ matrices $\bS$, where $r$ is a function of $n, d, \varepsilon,$ and $\delta$. 
Suppose that with probability at least $1 - \delta$, for any fixed $n \times p$ matrix $\bX$, a matrix $\bS$ drawn from distribution $P$ satisfies the property that $\bS$ is a $(1 \pm \varepsilon)$ $\ell_2$-subspace embedding for $\bX$. Then we call $P$ an \textit{$(\varepsilon, \delta)$ oblivious $\ell_2$-subspace embedding}.
\end{definition}

\citet{sarlos2006improved} then  proposed using Fast Johnson-Lindenstrauss transforms to construct  subspace embeddings.

\begin{definition}[Johnson-Lindenstrauss transform \citep{sarlos2006improved}]
A random matrix $\rmS \in \real^{k \times n}$ is said to form a \textit{Johnson-Lindenstrauss (JL) transform} with parameters $\varepsilon, \delta, f$, or JLT$(\varepsilon, \delta, f)$ for short, if with probability at least $1-\delta$, for any $f$-element subset $\mathcalV \subset \real^n$, 
the following holds for all $\bv, \bv' \in \mathcalV$:
$$
\abs{\innerproduct{ \bS\bv, \bS\bv'} - \innerproduct{\bv, \bv'}} \leq \varepsilon \normtwo{\bv} \normtwo{\bv'}.
$$
\end{definition}

If we set $\bv = \bv'$, this condition reduces to the familiar statement that $\normtwo{\bS\bv}^2 \in (1 \pm \varepsilon) \normtwo{\bv}^2$. It turns out that if we scale all $\bv, \bv' \in \mathcalV$ so that they are unit vectors, an equivalent condition can be formulated using only norms: specifically, we could require that 
$$
\normtwo{\bS\bv}^2 \in (1 \pm \varepsilon) \normtwo{\bv}^2
\qquad\text{and}\qquad 
\normtwo{\bS(\bv + \bv')}^2 \in (1 \pm \varepsilon) \normtwo{\bv + \bv'}^2
$$ for all $\bv, \bv' \in \mathcalV$. In other words, the definition can be equivalently stated in terms of vector norms rather than inner products.

There are several known constructions of Johnson-Lindenstrauss transforms. One of the simplest is provided by the Gaussian sketching theorem, which states that if the number of rows $k$ satisfies $k = \Omega(\varepsilon^{-2} \ln(f/\delta))$~\footnote{
Big Omega (resp. Theta) notation is used to provide an asymptotic lower (resp. tight) bound on the growth rate of a function. If $ f(n) $ and $ g(n) $ are two functions defined on the set of positive integers, then we say $ f(n) = \Omega(g(n)) $ if there exist positive constants $ c $ and $ n_0 $ such that:
$$
f(n) \geq c \cdot g(n) \quad \text{for all} \quad n \geq n_0
$$
This means that $ f(n) $ grows at least as fast as $ g(n) $ for sufficiently large values of $ n $. 
And we say $ f(n) = \Theta(g(n)) $ if there exist positive constants $ c_1 $, $ c_2 $, and $ n_0 $ such that:
$$
c_1 \cdot g(n) \leq f(n) \leq c_2 \cdot g(n) \quad \text{for all} \quad n \geq n_0
$$
This means that $ f(n) $ grows at the same rate as $ g(n) $ for sufficiently large values of $ n $. In other words, $ f(n) $ is both upper-bounded and lower-bounded by $ g(n) $ within constant factors.
}, 
then the resulting Gaussian matrix yields a valid JL transform.
\begin{theoremHigh}[Gaussian sketching \citep{indyk1998approximate, woodruff2014sketching}]\label{thm:gau_jlt}
Let $0 < \varepsilon, \delta < 1$, and let $\rmS = \frac{1}{\sqrt{k}} \rmG \in \real^{k \times n}$, where the entries $\rg_{ij}$ of $\rmG$ are independent standard normal random variables. Then, if $k = \Omega(\varepsilon^{-2} \ln(f/\delta))$, the matrix $\rmS$ is a JLT$(\varepsilon, \delta, f)$.
\end{theoremHigh}

\begin{theoremHigh}[Gaussian sketching \citep{woodruff2014sketching}]\label{theorem:gauss_ske_jl}
Let $0 < \varepsilon, \delta < 1$, and let $\rmS = \frac{1}{\sqrt{k}} \rmG \in \real^{k \times n}$, where the entries $\rg_{ij}$ are independent standard normal random variables. Then, if $k = \Theta((d + \ln(1/\delta))\varepsilon^{-2})$, then for any fixed $n \times p$ matrix $\bX$, with probability at least $1 - \delta$, $\rmS$ is a $(1 \pm \varepsilon)$ $\ell_2$-subspace embedding for $\bX$; that is, simultaneously for all $\bbeta \in \real^p$, it holds that $\normtwo{\rmS\bX\bbeta} \in (1 \pm \varepsilon) \normtwo{\bX\bbeta}$. 
\end{theoremHigh}

It turns out that Theorem~\ref{theorem:gauss_ske_jl} provides an optimal number of rows of $\rmS$ up to a constant factor---specifically, $\Theta(k\varepsilon^{-2})$. 

After Theorem~\ref{thm:gau_jlt} was introduced, several improvements followed. For example, \cite{achlioptas2003database} showed that one can replace the Gaussian matrix $\rmG$ in Theorem~\ref{thm:gau_jlt} with a matrix whose entries are i.i.d. sign random variables \cite{achlioptas2003database}; that is, each entry independently takes the value  $1$ or $-1$ with equal probability. Furthermore, he showed that the distribution can be modified so that each entry of $\rmG$ is set to $1$ with probability $1/6$,  $-1$ with probability $1/6$, and 0 with probability $2/3$. This modification results in a sparse matrix $\rmS$, which allows faster computation of the product $\rmS \cdot \bbeta$ for any vector $\bbeta \in \real^n$.

A significant advancement came from \citet{dasgupta2010sparse}, who showed that it suffices for each column of $\rmS$ to have only $\varepsilon^{-1}\text{poly}(\ln(f/\delta))$ nonzero entries. If the $\text{poly}(f/\delta)$ term is small compared to $\varepsilon^{-1}$, this represents a substantial improvement over earlier constructions, which required $\Omega(\varepsilon^{-2}\ln(f/\delta))$ nonzero entries per column. Later, \citet{kane2014sparser} improved this to $\mathcalO(\varepsilon^{-1}\ln(f/\delta))$ nonzero entries per column. This result was shown to be nearly tight by \citet{nelson2013sparsity}, who proved that at least $\Omega(\varepsilon^{-1}\ln(f/\delta)/\ln(1/\varepsilon))$  column sparsity is necessary.

In summary, this line of work shows that applying a JLT$(\varepsilon, \delta, f)$ matrix $\rmS$ to a vector $\bbeta$ can be done in time $\mathcalO(\nnz(\bbeta) \cdot \varepsilon^{-1}\ln(f/\delta))$, where $\nnz(\bbeta)$ denotes the number of nonzero entries in $\bbeta$. This leads to a significant speedup over Theorem~\ref{thm:gau_jlt} when $\varepsilon$ is small. It also improves upon Theorem~\ref{theorem:gauss_ske_jl}, although even better results are possible in the context of $\ell_2$-subspace embeddings \citep{woodruff2014sketching}.

Another approach aimed at speeding up the construction in Theorem~\ref{thm:gau_jlt} was proposed by \citet{ailon2006approximate}. Instead of focusing on sparsity, they sought to design matrices $\rmS$ that can be applied to vectors $\bbeta$ very efficiently. The key idea is that for a vector $\bbeta\in \real^n$ whose $\ell_2$ mass is approximately uniformly distributed across its coordinates, sampling a small number of coordinates uniformly at random and rescaling gives a good estimate of $\ell_2$-norm of $\bbeta$. However, if $\bbeta$ is sparse or has unevenly distributed mass, uniform sampling performs poorly, as most samples may be zero.

\index{Gaussian sketching}
\subsection{Gaussian Left Sketching}
We now discuss in more detail how Gaussian sketching can be applied to the least squares (LS) problem.

Let $\bX\in\real^{n\times p}$ with $n>p$.
Now we examine a particular formulation of $ \bS\in\real^{m\times n} $: the \textit{Gaussian left sketch}, or simply the \textit{Gaussian sketch}. That is, let $ \rmS $~\footnote{Note again that $\rmS$ (resp. $\rs$) denotes a random matrix (resp. random variable), and $\rmS=\bS$ (resp. $\rs=s$) denotes a realization of the random matrix (resp. random variable).} have i.i.d. Gaussian entries; specifically, each entry $ \rs_{ij} \sim \frac{1}{\sqrt{m}} \normal(0, 1) $ so that 
\begin{equation}\label{equation:mean_gauslef_i}
\Exp[\rmS^\top \rmS] = \bI 
\end{equation}
(since diagonals follow from a Chi-squared distribution; Definition~\ref{definition:chisquare_dist}). Recall that the LS and sketched LS solutions are given, respectively, by
$$
\widehat{\bbeta } = \argmin_{\bbeta \in \real^p} \normtwo{\bX \bbeta - \by}^2
\qquad \text{and}\qquad 
\widetildebbeta = \argmin_{\bbeta \in \real^p} \normtwo{\rmS \bX \bbeta - \rmS \by}^2.
$$
Note that the sources of randomness in these two estimators are different. The sampling distribution of $\widetildebbeta$ arises from the randomness in the sketching matrix $\rmS$, which induces variability through the sketched data $\rmS\bX$ and $\rmS\by$. In contrast, in Chapter~\ref{sec:lr-gaussian-noise}, the sampling distribution of $\widehatbbeta$ stems from the stochasticity in the response vector $\by$, typically due to additive noise in the model.

In other words, while $\widehatbbeta$ varies across different realizations of the noise in the data, $\widetildebbeta$ varies across different realizations of the random projection matrix $\rmS$, assuming the original data $(\bX,\by)$ is fixed.

\paragrapharrow{Mean under Gaussian sketching.}
Since $ \rmS $ is a random matrix, an important question is whether $ \Exp[\widetildebbeta] $ is equal to $ \widehatbbeta $, meaning that the expected value of the sketched solution equals the true least squares solution.
Assuming that $ \bX^\top \rmS^\top \rmS \bX $ is nonsingular (which will hold with high probability when $ m \gg p $), we can express the sketched solution as:
$$
\widetildebbeta = (\bX^\top \rmS^\top \rmS \bX)^{-1} \bX^\top \rmS^\top \rmS \by.
$$
Now, decompose $\by$ as $ \by = \bX \widehatbbeta + \by^\perp $, where $ \by^\perp \perp \cspace(\bX) $~\footnote{$\by^\perp$ is equivalent to the error vector $\be=\by-\bX\widehatbbeta$ we defined previously. Or equivalently $\by^\perp \in \nspace(\bX^\top)$, i.e., lies in the left null space; see the paragraph below Definition~\ref{definition:ortho_comp_col}. See Section~\ref{section:ls_fund_la} for more details.}. 
Substituting this decomposition into the expression for $\widetildebbeta$, we obtain:
\begin{equation}\label{equation:lef_gaus_ls_bb_dec}
\begin{aligned}
	\widetildebbeta 
	&= (\bX^\top \rmS^\top \rmS \bX)^{-1} \bX^\top \rmS^\top \rmS (\bX \widehatbbeta + \by^\perp)\\
	&= \widehatbbeta + (\bX^\top \rmS^\top \rmS \bX)^{-1} \bX^\top \rmS^\top \rmS \by^\perp.
\end{aligned}
\end{equation}
Then for $ \Exp[\widetildebbeta] = \widehatbbeta $ to hold, the term
$
\Exp[(\bX^\top \rmS^\top \rmS \bX)^{-1} \bX^\top \rmS^\top \rmS \by^\perp]
$
must vanish. To see that this will happen, note that $ \rmS \bX $ and $ \rmS \by^\perp $ are uncorrelated ($\bX^\top\by^\perp = \bzero$), which then implies independence. Then we can rewrite the expectation
\begin{equation}
\begin{aligned}
\Exp[\widetildebbeta] 
&= \widehatbbeta+\Exp_{\rmS \bX}[(\bX^\top \rmS^\top \rmS \bX)^{-1} \bX^\top \rmS^\top ] \cdot \Exp_{\rmS\by^\perp}[\rmS \by^\perp]\\
&= \widehatbbeta+\Exp_{\rmS \bX}[(\bX^\top \rmS^\top \rmS \bX)^{-1} \bX^\top \rmS^\top ] \cdot \bzero 
= \widehatbbeta.
\end{aligned}
\end{equation}
Thus, the expectation of the randomly projected solution matches the true least squares solution under Gaussian sketching.

\index{Sampling distribution}
\paragrapharrow{Variance under Gaussian sketching.}
We also analyze the variances of:
$$
\widetildebbeta
\qquad \text{and} \qquad 
\Exp\left[\normtwobig{\bX \widetildebbeta - \bX \widehatbbeta}^2\right].
$$
One key observation is that the variance tends to be smaller when the objective value $ f(\widehatbbeta) \triangleq \normtwobig{\bX \widehatbbeta - \by}^2 $ is small. This is analogous to having low variance when the training loss is small in machine learning models.

Another thing to note is that throughout, we assume that $ \bX^\top \bS^\top \bS \bX $ and $\bX^\top\bX$ are nonsingular. 
In fact, if the eigenvalues of this matrix are small, i.e., $ \bX^\top \bS^\top \bS \bX $ is nearly singular, then the variance will increase accordingly. However, we have some control over this via the construction of $ \bS $.

To analyze the variance, we begin by conditioning on on $\bS \bX$ to derive a conditional distribution, and then relax this assumption to obtain a full characterization.

\paragraph{Fixing $\rmS\bX=\bS\bX$.} Note that by fixing $\bS \bX$, the only source of randomness in $\widetildebbeta$ comes from $\rmS\by^\perp$. 
We can express this as:
$$
\rmS\by^\perp 
%= 
%\begin{bmatrix}
%	\bS_1^\top \by^\perp \\
%	\bS_2^\top \by^\perp \\
%	\vdots
%\end{bmatrix} 
=
\begin{bmatrix}
	\sum_j \rs_{1j} y_j^\perp \\
	\sum_j \rs_{2j} y_j^\perp \\
	\vdots\\
	\sum_j \rs_{mj} y_j^\perp 
\end{bmatrix}.
$$
Since $\rs_{ij} \sim \frac{1}{\sqrt{m}}\normal(0,1)$, it follows that 
$$
\Var\Big[\sum_j \rs_{ij} y_j^\perp\Big] = \sum_j \left(y_j^\perp\right)^2 \cdot \frac{1}{m}  
= \normtwo{\by^\perp}^2 \cdot \frac{1}{m}
= \frac{1}{m}\normtwo{\by - \bX\widehatbbeta}  
\triangleq \frac{1}{m} f(\widehatbbeta) .
$$
Therefore, we have that $\rmS\by^\perp \sim \normal\left(\bzero, \frac{f(\widehatbbeta)}{m} \bI\right)$. And by \eqref{equation:lef_gaus_ls_bb_dec} and Lemma~\ref{lemma:affine_mult_gauss}, it follows that:
\begin{equation}\label{equation:leftgauss_tildebbe}
\widetildebbeta \sim \normal\left(\widehatbbeta, \frac{f(\widehatbbeta)}{m} (\bX^\top \bS^\top \bS \bX)^{-1}\right),
\qquad \text{when fixing }\bS\bX,
\end{equation}
from which it follows that
\begin{equation}\label{equation:dist_xbb_fixsx}
\bX (\widetildebbeta - \widehatbbeta) \sim \normal\left(\bzero, \frac{f(\widehatbbeta)}{m} \bX \left(\bX^\top \bS^\top \bS \bX\right)^{-1} \bX^\top\right),
\qquad \text{when fixing }\bS\bX.
\end{equation}

\paragraph{Random $\rmS\bX$.} 
Now suppose  $\rmS \bX$ is no longer fixed. 
Although $\Exp[\bX^\top \rmS^\top \rmS \bX]$ is an unbiased estimator of $\bX^\top \bX$, 
the estimate $\Exp[(\bX^\top \rmS^\top \rmS \bX)^{-1}]$ introduces bias when estimating $(\bX^\top \bX)^{-1}$, with  a factor $\frac{m}{m-p-1}$. 
That is, for $m > p + 1$, we have
\begin{equation}\label{equation:inv_xssx}
\Exp\left[\left(\bX^\top \rmS^\top \rmS \bX\right)^{-1}\right] = \left(\bX^\top \bX\right)^{-1} \frac{m}{m-p-1}.
\end{equation}
Therefore, if $m = p - 1$ for instance, the variance in \eqref{equation:dist_xbb_fixsx} will blow up.

\index{Singular value decomposition}
\begin{proof}[of \eqref{equation:inv_xssx}]
Let $\rmZ\triangleq \rmS\bX$ with rows $\rvz_1, \rvz_2, \ldots, \rvz_m\sim \normal (\bzero, \frac{1}{m}\bX^\top\bX)$ (Lemma~\ref{lemma:affine_mult_gauss}).
Therefore, $\bZ^\top\bZ\sim \wishartdist(\frac{1}{m}\bX^\top\bX, m)$ and thus $(\bZ^\top\bZ)^{-1}\sim\inversewishart\big((\frac{1}{m}\bX^\top\bX)^{-1}, m\big)$ (Definitions~\ref{definition:wishart_dist} and \ref{definition:multi_inverse_wishart}). This completes the proof.
\end{proof}

Note that for $\rvz \sim \normal(\bzero, \bZ)$, since $\trace(\bZ) = \trace(\Exp[\rvz\rvz^\top])$, it follows that
$
\Exp[\normtwo{\rvz}^2] = \Exp[\trace(\rvz\rvz^\top)] = \Exp[\trace(\bZ)]
$, 
whence we have 
$$
\small
\begin{aligned}
\Exp \normtwo{\bX (\widetildebbeta - \widehatbbeta)}^2 
= \Exp\left[\frac{f(\widehatbbeta)}{m} \trace \left(\bX \left(\bX^\top \rmS^\top \rmS \bX\right)^{-1} \bX^\top\right)\right]
= \frac{f(\widehatbbeta)}{m-p-1} \trace \big(\bX \left(\bX^\top \bX\right)^{-1} \bX^\top\big).
\end{aligned}
$$~\footnote{By $\Exp\normtwo{\cdot}^2$ we mean the expectation of $\normtwo{\cdot}^2$ rather than the square of $\Exp[\normtwo{\cdot}]$.}
To see the specific value of $\trace \big(\bX \left(\bX^\top \bX\right)^{-1} \bX^\top\big)$,
notice that $ \bX \left( \bX^\top \bX \right)^{-1} \bX^\top = \bX \bX^+ $ projects onto the column space of $ \bX $ (Lemma~\ref{lemma:column-space-of-projection}). Let $\bX$ admit the reduced SVD $ \bX = \bU \bSigma \bV^\top $ (Figure~\ref{fig:svd-comparison}).
Since we assume $\bX^\top\bX$ is nonsingular (i.e., $\rank(\bX)=p$ when $n>p$),
we have 
$$
\begin{aligned}
\trace (\bX \big( \bX^\top \bX \big)^{-1} \bX^\top )
&= \trace (\bU \bU^\top )
\stackrel{\dag}{=} \trace (\bU^\top \bU )
= \trace (\bI_{p})
= p
= \rank(\bX),
\end{aligned}
$$
where the equality ($\dag$) follows from the cyclic invariance of traces.
So we conclude that
\begin{equation}\label{equation:vari_xbtilde_bhat}
\Exp \normtwo{\bX (\widetildebbeta - \widehatbbeta)}^2 = f(\widehatbbeta) \frac{p}{m - p - 1}.
\end{equation}
where $ f(\widehatbbeta) \frac{m}{m - p - 1} $ provides insight into how to choose $m$ when constructing $\bS$ in order to achieve a desired expected error.

\begin{exercise}\label{exercise:left_gauss_xtildebminusy}
Let $\bX\in\real^{n\times p}$ with full column rank $p$.
Using the results in \eqref{equation:leftgauss_tildebbe}, \eqref{equation:dist_xbb_fixsx}, and \eqref{equation:vari_xbtilde_bhat} to show that 
$$
\Exp\normtwo{\bX\widetildebbeta - \by}^2 = f(\widehatbbeta)\frac{m-1}{m - p - 1}.
$$
When $\rank(\bX)=r<p$, show that 
$$
\Exp\normtwo{\bX\widetildebbeta - \by}^2 = f(\widehatbbeta)\frac{m-1}{m - r - 1}.
$$
\end{exercise}

\index{Singular value decomposition}
\subsection{Special Sketching Matrices}\label{section:speci_lef_ske}

We observed that when using random sketching, certain conditions must be satisfied to guarantee approximate optimality (e.g., recall that the Gaussian sketch requires $ \bX^\top \bS^\top \bS \bX $ to be invertible). We can also explore deterministic constructions for the sketching matrix $ \bS $. 
Let $ \bX = \bU \bSigma \bV^\top $ be the reduced SVD of $ \bX $, where $\bU\in\real^{n\times p},\bSigma\in\real^{p\times p}$, and $\bV\in\real^{p\times p}$ if $\bX$ has full rank with $n>p$.
In this context, two notable choices for the sketching matrix $\bS$ arise:
\paragrapharrow{Option 1: $ \bS = \bU^\top $.}
Suppose we choose $ \bS = \bU^\top $, i.e., the matrix containing the left singular vectors in its rows as our sketching matrix.
In this case, we have
$$
\begin{aligned}
\widetildebbeta 
&= (\bS \bX)^+ \bS \by 
= (\bU^\top \bU \bSigma \bV^\top)^+ \bS \by
= (\bSigma \bV^\top)^+ \bS \by \\
&= \bV \bSigma^{-1} \bS \by
= \bV \bSigma^{-1} \bU^\top \by 
= \bX^+ \by
= \widehatbbeta.
\end{aligned}
$$
Thus, by choosing $ \bS = \bU^\top $, we exactly recover the least squares solution. However, this approach requires computing the left singular vectors of $\bX$, which takes $\mathcalO(n p^2) $ time. As a result, there is no computational advantage compared to solving the original (unsketched) least squares problem.

\paragrapharrow{Option 2: $ \bS = \bX^\top $.}
Alternatively, suppose we choose $ \bS = \bX^\top $. Then we obtain:
$$
\begin{aligned}
\widetildebbeta 
&= (\bS \bX)^+ \bS \by
= (\bX^\top \bX)^+ \bX^\top \by 
= \bV \bSigma^{-2} \bV^\top \bV \bSigma \bU^\top \by \\
&= \bV \bSigma^{-1} \bU^\top \by 
= \widehatbbeta.
\end{aligned}
$$
Again, this choice leads to an exact recovery of the least squares solution. However, as in the previous case, it requires computing the pseudo-inverse of  $ \bX^\top \bX $, which also takes  $\mathcalO(n p^2) $ time. Consequently, this method again offers no computational savings compared to the standard least squares solution.

\index{Variance reduction}
\subsection{Variance Reduction by Averaging}

In the previous sections, we showed that the average deviation of the left-sketched least squares solution from the original solution is proportional to $ f(\widehatbbeta) $.
Since this deviation corresponds to the variance of the random variable $ \bX(\widetildebbeta-\widehatbbeta) $, we can reduce this variance—and thus improve our estimate of $ f(\widetildebbeta) $---by averaging over multiple i.i.d. instances of $\widetildebbeta$.

Let $ \rmS_1,\rmS_2, \ldots, \rmS_q $ be independent sketching matrices, each of which has entries drawn independently from a scaled normal distribution $ \frac{1}{\sqrt{m}} \normal(0, 1) $, such that $\Exp[\rmS_i^\top \rmS_i] = \bI$, $\forall\, i\in\{1,2,\ldots,q\}$. We define each sketched estimator $\widetildebbeta_i$ as follows:
$$
\widetildebbeta_i = \argmin_{\bbeta \in \real^p} \normtwo{\bS_i \bX \bbeta - \bS_i \by}^2 ,
\qquad \forall\, i\in\{1,2,\ldots,q\}.
$$
Now define the averaged estimator $\overline{\bbeta}$ as
$
\overline{\bbeta} \triangleq \frac{1}{q} \sum_{i=1}^q \widetildebbeta_i.
$
This estimator $\overline{\bbeta}$ is unbiased because
$$
\Exp[\overline{\bbeta}] = \Exp \left[ \frac{1}{q} \sum_{i=1}^q \widetildebbeta_i \right] = \frac{1}{q} \sum_{i=1}^q \Exp[\widetildebbeta_i] = \frac{1}{q} \sum_{i=1}^q \widehatbbeta = \widehatbbeta.
$$
Additionally, using the fact that the $\widetildebbeta_i$'s are independent, the variance is reduced by a factor $\frac{1}{q}$. Specifically:
\begin{equation}
\Exp[f(\overline{\bbeta}) - f(\widehatbbeta)] 
= \Exp \normtwo{\bX (\overline{\bbeta} - \widehatbbeta)}^2 
= \frac{1}{q} f(\widehatbbeta) \frac{p}{m - p - 1}.
\end{equation}

The computational complexity of this averaging algorithm becomes $\mathcalO(q\cdot m p^2)$.
On the other hand, when employing the non-averaging algorithm with $\bS\in\real^{qm\times n}$, the computational complexity requires $\mathcalO(q\cdot m p^2)$, and the expected error becomes:
\begin{equation}
	\Exp[f(\widetildebbeta) - f(\widehatbbeta)] 
	= \Exp \normtwo{\bX (\widetildebbeta - \widehatbbeta)}^2 
	=  f(\widehatbbeta) \frac{p}{qm - p - 1}.
\end{equation}
Since $f(\widehatbbeta) \frac{p}{qm - p - 1} < \frac{1}{q} f(\widehatbbeta) \frac{p}{m - p - 1}$, the non-averaged method achieves better accuracy at the same computational cost. Therefore, the averaging algorithm is rarely used in practice.
However, in parallel or distributed computing environments, or on devices with limited memory, the averaging approach may still offer practical advantages due to its modular and lightweight nature.

\subsection{Sketched Matrix Least Squares}\label{section:sketc_mat_ls}

More generally, 
we consider the matrix least squares (a.k.a., \textit{multiple-response least squares}) problem, which is formulated as:
\begin{equation}\label{equation:matrix_ls}
	[\widehatbbeta_1,\widehatbbeta_2, \ldots, \widehatbbeta_q ]\triangleq \widehatbB = \argmin_{\bB \in \real^{p \times q}} \normf{{\bX}\bB - \bY}^2,
\end{equation}
where $\bX \in \real^{n\times p}$ and $\bY \in \real^{n\times q}$. 
The (left-) sketched variant of \eqref{equation:matrix_ls}    is given by,
\begin{equation}\label{equation:matrix_ls_ske}
	[\widetildebbeta_1,\widetildebbeta_2, \ldots, \widetildebbeta_q ]\triangleq	\widetilde{\bB} 
	= \argmin_{\bB \in \real^{p \times q}} \normf{\bS \bX \bB - \bS\bY}^2,
\end{equation}
where $\bS \in \real^{m \times n}$ is a (Gaussian) sketching matrix. 
The solution of \eqref{equation:matrix_ls_ske} can be obtained in closed form as 
$\widetilde{\bB} = (\bS\bX)^+ \bS\bY$ (Theorem~\ref{theorem:svd-deficient-rank}).

When $q=1$, this formulation reduces to the standard sketched least squares problem.
When $q>1$,
the original problem and sketched variant are
$$
\widehatbB \triangleq \argmin_{\bB} \normf{\bX\bB - \bY}^2 
\qquad \text{and} \qquad 
\widetildebB \triangleq \argmin_{\bB} \normf{\bS \bX\bB - \bS\bY}^2.
$$
The $i$-th column of $\widetildebB$ satisfies
$
\widetildebbeta_i = \argmin_{\bbeta_i} \normtwo{\bS \bX\bbeta_i - \bS \by_i}^2.
$
For a Gaussian sketching matrix $\bS$, we have
$ \Exp \normtwobig{\bX(\widetildebbeta_i - \widehatbbeta_i)}^2 = \normtwobig{\bX\widehatbbeta_i - \by_i}^2 \frac{p}{m-p-1}$, 
which implies
$$
\Exp \normf{\bX(\widetildebB - \widehatbbeta)}^2 = \sum_{i=1}^q \normtwo{\bX\widehatbbeta_i - \by_i}^2 \frac{p}{m-p-1}
= \normf{\bX\widehatbbeta - \bY}^2 \frac{p}{m-p-1}.
$$
Suppose that $\rank(\bX) = r$. By Exercise~\ref{exercise:left_gauss_xtildebminusy}, we then have 
\begin{align}
	\Exp \normf{\bX(\widetildebB - \widehatbB)}^2 
	&= \normf{\bX\widehatbB - \bY}^2 \frac{r}{m-r-1};\\
	\Exp \normf{\bX\widetildebB - \bY}^2 
	%	&= \Exp \normf{\bX\widehatbB - \bY + \bX(\widetildebB - \widehatbB)}^2\\
	%	&= \normf{\bX\widehatbB - \bY}^2 + \Exp \normf{\bX(\widetildebB - \widehatbB)}^2
	%	= \normf{\bX\widehatbB - \bY}^2 \left(1 + \frac{r}{m-r-1}\right)\\
	&= \normf{\bX\widehatbB - \bY}^2 \frac{m-1}{m-r-1}.\label{equation:mls_lske_mainres2}
\end{align}
This result is referred to as the \textit{left sketching optimality gap} under the Gaussian sketch \citep{pilanci29lecture, halko2011finding}.

\index{High-dimensional LS}
\section{Sketched High-Dimensional Least Squares}

We now consider the problem of solving the high-dimensional linear system $\bX\bbeta = \by$, where $\bX \in \real^{n \times p}$ and $p > n$ (Section~\ref{section:ls-via-svd}). In general, such a system does not have a unique solution because there are more unknowns than equations. However, among all possible solutions, the minimum ($\ell_2$) norm solution is typically unique and well-defined; see Section~\ref{section:ls_conv}. 
The minimum-norm solution is defined as follows:
$$
\widehatbbeta_{\text{mn}} = \argmin_{\bX\bbeta=\by} \normtwo{\bbeta}^2
=
 \bX^+\by= \bX^\top (\bX\bX^\top)^{-1} \by.
$$
Similar to the least squares case, we can reduce the dimensionality of the problem by multiplying $\bX$ on the right by a random projection matrix $\bS \in \real^{p \times m}$  to form $\bX\bS$, where  $p>m$ (in most scenarios, we only consider $p>m>n$), and solve
\begin{equation}
\argmin_{\bX\bS \balpha=\by} \normtwo{\balpha}^2. 
\end{equation}

By right-multiplying $\bX$ by $\bS$, we change the dimension of the optimization variable. Specifically, the vector being minimized---$\balpha$---now has a smaller dimension than $\bbeta$. To address this, we use the relation $\bbeta=\bS\balpha$, and hope that $\bS\balpha$ provides a good approximation to the original solution $\bbeta$. As we will see, this is often the case.

Let
$ \widetilde{\balpha} \triangleq \argmin_{\bX\bS \balpha=\by} \normtwo{\balpha}^2$
and use the approximation that $\widetildebbeta \triangleq  \bS\widetilde{\balpha}$. A solution to $\widetilde{\balpha}$ is $(\bX\bS)^+\by$ (see Section~\ref{section:ls-via-svd}). Substituting into the constraint equation, when $\bX\bS$ has full row rank, we have that 
$$
\begin{aligned}
\bX\widetildebbeta 
&= \bX\bS\widetilde{\balpha}
= \bX\bS(\bX\bS)^+\by
= \by.
\end{aligned}
$$
This shows that $\widetildebbeta = \bS\widetilde{\balpha}$ is a valid solution satisfying the original constraint $\bX\bbeta = \by$, provided that $\bX\bS$ has full row rank.
This also shows that the error vector 
$$
\widetildebbeta - \widehatbbeta_{\text{mn}} \in \nspace(\bX).
$$

\index{Gaussian sketching}
\index{Minimum-norm solution}
\index{Singular value decomposition}
\subsection{Gaussian Right Sketching}
Let $\rmS\in\real^{p\times m}$ with entries  $\rs_{ij} \overset{\text{iid}}{\sim}\frac{1}{\sqrt{m}} \normal(0, 1)$ be the right sketching matrix, and let  $\widetilde{\balpha}$ denote the solution to the sketched minimum-norm problem. Our estimate for $\widehatbbeta_{\text{mn}}$, defined as $\widetildebbeta = \bS\widetilde{\balpha}$, possesses several important properties.
Similarly, we begin by assuming that $\bX\bS$ is a fixed matrix.
\begin{lemma}[Distribution of $\widetildebbeta$ under fixed $\bX\bS$]\label{lemma:mean_hidd_ls_gaussske}
Let $\bX\in\real^{n\times p}$ have full row rank with $n<p$. For a fixed $\bX\bS\in \real^{n\times m}$ of full row rank, 
it follows that 
$$
\widetildebbeta \sim \normal\left(\widehatbbeta_{\text{mn}}, \frac{1}{m} \by^\top (\bX \bS \bS^\top \bX^\top)^{-1} \by\bI\right).
$$
That is, $\widetildebbeta$ is an unbiased estimator of $\widehatbbeta_{\text{mn}}$, i.e., $\Exp[\widetildebbeta] = \widehatbbeta_{\text{mn}}$.
\end{lemma}
\begin{proof}[of Lemma~\ref{lemma:mean_hidd_ls_gaussske}]
Let $\bX=\bU\bSigma\bV_1^\top$ be the reduced SVD of $\bX$, where $\bU,\bSigma\in\real^{n\times n}$, and $\bV_1\in\real^{p\times n}$.
Consider further the full right singular vector matrix $[\bV_1, \bV_2]\in\real^{p\times p}$, where $\bV_1 \in \real^{p\times n}$ and $\bV_2 \in \real^{p\times (p-n)}$. Left-multiplying $\widetildebbeta$ by $\bV_1^\top$ yields that 
$$
\begin{aligned}
\bV_1^\top \widetildebbeta 
&=\bV_1^\top \bS \widetilde{\balpha}
=\bV_1^\top \bS (\bX\bS)^+ \by 
=\bV_1^\top \bS \bS^\top \bX^\top (\bX\bS\bS^\top \bX^\top)^{-1} \by \\
&=\bV_1^\top \bS \bS^\top \bV_1 \bSigma \bU^\top (\bU \bSigma \bV_1^\top \bS \bS^\top \bV_1 \bSigma \bU^\top)^{-1} \by \\
&=\bV_1^\top \bS \bS^\top\bV_1 \bSigma \bU^\top \bU \bSigma^{-1} (\bV_1^\top \bS \bS^\top\bV_1)^{-1} \bSigma^{-1} \bU^\top \by  \\
&=\bV_1^\top \bS \bS^\top \bV_1 (\bV_1^\top \bS \bS^\top\bV_1)^{-1} \bSigma^{-1} \bU^\top \by \\
&= \bSigma^{-1} \bU^\top \by
=\bV_1^\top \bV_1 \bSigma^{-1} \bU^\top \by 
=\bV_1^\top \widehatbbeta_{\text{mn}}.
\end{aligned}
$$
Similarly, left-multiplying $\widetildebbeta$ by $\bV_2^\top$ yields that 
$$
\begin{aligned}
\bV_2^\top \widetildebbeta 
&=\bV_2^\top \bS \bS^\top \bV_1 (\bV_1^\top \bS \bS^\top\bV_1)^{-1} \bSigma^{-1} \bU^\top \by \\
&=\bV_2^\top \bS \bS^\top \bV_1 (\bV_1^\top \bS \bS^\top\bV_1)^{-1}\bV_1^\top \bV_1 \bSigma^{-1} \bU^\top \by \\
&=\bV_2^\top \bS \bS^\top \bV_1 (\bV_1^\top \bS \bS^\top\bV_1)^{-1}\bV_1^\top \widehatbbeta_{\text{mn}}.
\end{aligned}
$$
Taking the expectation of $\bV_2^{\top} \widetildebbeta$:
$$
\begin{aligned}
	\Exp[\bV_2^\top \widetildebbeta] 
	&= \Exp[\bV_2^\top \bS \bS^\top \bV_1 (\bV_1^\top \bS \bS^\top\bV_1)^{-1}\bV_1^\top \widehatbbeta_{\text{mn}}] \\
	&= \Exp[\bV_2^\top \bS ]  \cdot \bS^\top \bV_1 (\bV_1^\top \bS \bS^\top\bV_1)^{-1}\bV_1^\top \widehatbbeta_{\text{mn}} 
	= \bzero,
\end{aligned}
$$
where the penultimate equality follows since $\bS^\top\bV_2$ and $\bS^\top\bV_1$ are uncorrelated, and   $\Exp[\bV_2^\top \bS ]=\bzero$ follows since $\rs_{ij} \overset{\text{iid}}{\sim} \frac{1}{\sqrt{m}}\normal(0, 1)$.
Now consider the expectation of the orthogonal matrix-vector product $[\bV_1, \bV_2]^\top \widetildebbeta$, we get that
$$
\Exp \left[[\bV_1, \bV_2]^\top \widetildebbeta\right]
= \begin{bmatrix}
	\Exp[\bV_1^\top \widetildebbeta] \\
	\Exp[\bV_2^\top \widetildebbeta]
\end{bmatrix}
= \begin{bmatrix}
	\bV_1^\top \widehatbbeta_{\text{mn}} \\
	\bzero
\end{bmatrix}
\equiv 
[\bV_1, \bV_2]^\top\widehatbbeta_{\text{mn}}.
$$
where the last equality follows because $\bV_2^\top \widehatbbeta_{\text{mn}} =  \bV_2^\top \bX^+\by= \bV_2^\top\bX^\top (\bX\bX^\top)^{-1} \by=\bzero$.
Because the expectation of $\widetildebbeta$ multiplied by an orthogonal matrix is equal to $\widehatbbeta_{\text{mn}}$ multiplied by the same orthogonal matrix, we can conclude that $\Exp[\widetildebbeta] = \widehatbbeta_{\text{mn}}$.

For the variance, let $\rvs_i$ be the $i$-th row of the random matrix $\rmS$. 
Since we fix $\bX\bS$  and $\widetildebbeta =\rmS \widetilde{\balpha} 
=\bS (\bX\bS)^+\by = \rmS (\bX\bS)^\top (\bX\bS\bS^\top\bX^\top)^{-1}\by$, it follows from Lemma~\ref{lemma:affine_mult_gauss} that 
$$
\widetilde{\beta}_i = \rvs_i^\top (\bX\bS)^\top (\bX\bS\bS^\top\bX^\top)^{-1}\by
\quad\implies\quad 
\Var[\widetilde{\beta}_i] = \frac{1}{m} \by^\top (\bX \bS \bS^\top \bX^\top)^{-1} \by.
$$
Therefore, the covariance matrix of $\widetildebbeta$ is given by $\Cov[\widetildebbeta] = \frac{1}{m} \by^\top (\bX \bS \bS^\top \bX^\top)^{-1} \by\bI$.
\end{proof}

In Lemma~\ref{lemma:mean_hidd_ls_gaussske}, we assume that $\bX\bS$ is fixed rather than a random matrix. When $\bX\rmS$ is instead a random matrix, as opposed to the variance of $\bX (\widetildebbeta - \widehatbbeta)$ given in \eqref{equation:vari_xbtilde_bhat}, we can derive the variance of $\widetildebbeta - \widehatbbeta_{\text{mn}}$.
\begin{theoremHigh}[Variance of $\widetildebbeta - \widehatbbeta_{\text{mn}}$]\label{theorem:var_wideb_bmn}
Let $\bX\in\real^{n\times p}$ have full row rank with $n<p$. For a random $\bX\rmS\in \real^{n\times m}$ with full row rank, 
it follows that 
$$
\Exp\normtwo{\widetildebbeta - \widehatbbeta_{\text{mn}}}^2 = \frac{p}{m-n-1} \normtwo{\widehatbbeta_{\text{mn}}}^2.
$$
\end{theoremHigh}
\begin{proof}[of Theorem~\ref{theorem:var_wideb_bmn}]
Let $\bX=\bU\bSigma\bV_1^\top$ be the reduced SVD of $\bX$, where $\bU,\bSigma\in\real^{n\times n}$, and $\bV_1\in\real^{p\times n}$.
Note that for $\rvz \sim \normal(\bzero, \bZ)$, since $\trace(\bZ) = \trace(\Exp[\rvz\rvz^\top])$, it follows that
$
\Exp[\normtwo{\rvz}^2] = \Exp[\trace(\rvz\rvz^\top)] = \Exp[\trace(\bZ)]
$.
We then have 
$$
\begin{aligned}
\Exp\normtwobig{\widetildebbeta - \widehatbbeta_{\text{mn}}}^2 
&= \Exp\big[(\widetildebbeta - \widehatbbeta_{\text{mn}})^\top (\widetildebbeta - \widehatbbeta_{\text{mn}})\big] 
= \Exp\big[\trace (\widetildebbeta - \widehatbbeta_{\text{mn}})(\widetildebbeta - \widehatbbeta_{\text{mn}})^\top\big] \\
&= \Exp\left[\trace \left(\frac{1}{m}\by^\top (\bX\bS\bS^\top \bX^\top)^{-1} \by\cdot \bI\right)\right]
= \frac{1}{m} \by^\top \Exp[(\bX\bS\bS^\top \bX^\top)^{-1}] \by \cdot \trace (\bI) \\
&\stackrel{\dag}{=} \frac{1}{m} \frac{m}{m-n-1} \cdot p (\by^\top (\bX\bX^\top)^{-1} \by) 
= \frac{p}{m-n-1} \by^\top (\bX\bX^\top)^{-1} \by \\
&= \frac{p}{m-n-1} \by^\top \bU \bSigma^{-2} \bU^\top \by 
= \frac{p}{m-n-1} (\bV_1 \bSigma^{-1} \bU^\top \by)^\top \bV_1 \bSigma^{-1} \bU^\top \by \\
&= \frac{p}{m-n-1} \normtwo{\widehatbbeta_{\text{mn}}}^2,
\end{aligned}
$$
where the last equality follows from $\widehatbbeta_{\text{mn}}=\bX^+\by= \bX^\top (\bX\bX^\top)^{-1} \by = \bV_1 \bSigma^{-1} \bU^\top \by$, and the equality ($\dag$) follows from the  fact that $\Exp[(\bX\bS\bS^\top \bX^\top)^{-1}] = (\bX\bX^\top)^{-1} \frac{m}{m-n-1}$ (the proof of which is similar to that of \eqref{equation:inv_xssx}).
This completes the proof.
\end{proof}

\subsection{Special Sketching Matrices}

So far, we have assumed that $ \rmS $ is an i.i.d. Gaussian matrix in order to simplify the theoretical analysis. However, many other types of sketching matrices are possible. 
Recall that $ \bX $ can be written as $ \bX = \bU \bSigma \bV_1^\top $ in its reduced SVD form, where $\bU\in\real^{n\times n},\bSigma\in\real^{n\times n}$, and $\bV_1\in\real^{p\times n}$ if $\bX$ has full row rank with $p>n$, and that the minimum-norm solution to $ \bX\bbeta = \by $ is given by $ \widehatbbeta_{\text{mn}} =\bX^+\by=\bX^\top (\bX\bX^\top)^{-1} \by $. 
Similar to the left sketching technique described in Section~\ref{section:speci_lef_ske}, the following deterministic choices for $\bS$ are worth considering:

\paragrapharrow{Option 1: $ \bS = \bV_1 \in\real^{p\times n} $.}
The minimum-norm solution using this sketching matrix corresponds to the left pseudo-inverse of $  \bX \bS$, i.e.,
$$
\begin{aligned}
\widetildebbeta &=  \bS\widetilde{\balpha}=  \bS (\bX\bS)^+\by= \bV_1^\top (\bU\bSigma\bV_1^\top\bV_1)^{-1}\by = \bV_1^\top\bSigma^{-1}\bU^\top\by = \bX^+\by=\widehatbbeta.
\end{aligned}
$$
\paragrapharrow{Option 2: $ \bS = \bX^\top $.}
The minimum-norm solution for this sketching matrix is:
$$
\begin{aligned}
\widetildebbeta &= \bS\widetilde{\balpha}= \bS(\bX\bS)^+\by =\bX^\top (\bX\bX^\top)^{-1}\by = \widehatbbeta.
\end{aligned}
$$

These choices of $\bS$ yield exact solutions to the minimum-norm LS problem. However, they are equivalent to classical direct methods for solving the same problem, which contradicts the original purpose of sketching---namely, to approximately solve the optimization problem more efficiently using a JL embedding or similar dimensionality reduction technique.

\section{Sketched Least Squares with Quantized Response}\label{section:ske_qu_ls}

In many applications, the dimension $q$ of the matrix least squares problem introduced in Section~\ref{section:sketc_mat_ls} can be extremely large, making it difficult to store the response matrix $\bY$, especially on devices with limited memory or computational resources. One possible solution is to store a quantized version of $\bY$. Fortunately, it turns out that the least squares solution obtained using Gaussian sketching with quantized responses still concentrates around the true solution; that is, the corresponding residual matrix is upper bounded \citep{saha2023matrix}.

\index{Quantized response}
\index{Uniformly dithered quantizer}
\index{Quantization}
\paragrapharrow{Uniformly dithered quantizer.}
We begin by considering the quantization of a scalar $\beta$ satisfying $\abs{\beta} \leq R$. Given a bit-budget of $B$ bits, the scalar quantizer with dynamic range $R$ is defined by first setting the $M \triangleq 2^B$ quantization levels as:
$$
q_1 = -R,\quad q_2 = -R + \Delta,\quad q_3 = -R + 2\Delta,\quad \ldots, \quad q_M = -R + (M-1)\Delta,
$$
where the \textit{resolution} is given by $\Delta = \frac{2R}{M-1}$.
The operation of the \textit{uniformly dithered quantizer} is then defined as:
\begin{equation}\label{equation:dither_quantizer_prob}
Q_{R,B}(\beta) = 
\begin{cases} 
q_{k+1}, & \text{with probability } r \triangleq \frac{\beta - q_k}{\Delta}, \\
q_k, & \text{with probability } 1-r,
\end{cases}
\end{equation}
where $k = \arg \max_j \{q_j \leq \beta\}$, meaning $\beta \in [q_k, q_{k+1}]$. 
If the input $\beta$ to the quantizer lies outside this interval, i.e., $\beta > R$ or $\beta < -R$, the quantizer is said to be \textit{saturated}. 

Finally, to quantize any matrix $\bB=\{\beta_{ij}\}$, we obtain $Q_{R,B}(\bB)$ by quantizing each entry independently, i.e., $[Q_{R,B}(\bB)]_{ij} \triangleq Q_{R,B}(\beta_{ij})$.
Then, we have the following result about the unbiasedness of the uniformly dithered quantizer.
\begin{lemma}[Uniformly dithered quantizer]\label{lemma:unif_dit_qua}
Consider the quantization of a scalar  $\beta \in [-R, +R]$, and denote the quantization error of uniformly dithered scalar quantizer with a bit-budget of $B$ bits as $\epsilon \triangleq Q_{R,B}(\beta) - \beta$.
Clearly, the quantization error is bounded, satisfying $\abs{\epsilon} \leq \Delta$.
Then it follows that 
$$
\Exp[\epsilon] = 0 \qquad \text{and} \qquad \Var[\epsilon] \leq \frac{\Delta^2}{4} =\frac{R^2}{(2^B - 1)^2},
$$
where the expectation $\Exp[\cdot]$ is taken over the randomness introduced by the dithering in the quantizer operation.
Hence, the uniformly dithered quantizer is unbiased, and the variance of the quantization error depends only on $R$ and $B$.
\end{lemma}
\begin{proof}[of Lemma~\ref{lemma:unif_dit_qua}]
Suppose $\beta \in [q_k, q_{k+1}]$, where $q_{k+1} = q_k + \Delta$ and $\Delta = \frac{2R}{2^B - 1}$. Then,
$$
\Exp[Q_{R,B}(\beta)] = q_{k+1} \frac{\beta - q_k}{\Delta} + q_k \left(1 - \frac{\beta - q_k}{\Delta}\right) 
= \frac{(q_k + \Delta)(\beta - q_k) + q_k (\Delta - \beta + q_k)}{\Delta} = \beta.
$$
Now consider the variance:
$$
\begin{aligned}
\Var[\epsilon]=\Exp[(Q_{R,B}(\beta) - \beta)^2 ]
&= (q_{k+1} - \beta)^2 \left(\frac{\beta - q_k}{\Delta}\right) + (q_k - \beta)^2 \left(\frac{q_{k+1}-\beta}{\Delta}\right)\\
&= (q_{k+1} - \beta)(\beta - q_k)
\leq \max_{\beta \in [q_k, q_{k+1}]} (q_{k+1} - \beta)(\beta - q_k)\\
&= \left(q_{k+1} - \frac{q_k + q_{k+1}}{2}\right) \left(\frac{q_k + q_{k+1}}{2} - q_k\right) = \frac{\Delta^2}{4} = \frac{R^2}{(2^B - 1)^2}.
\end{aligned}
$$
This completes the proof.
\end{proof}

\paragrapharrow{Sketched least squares with quantized response.}
The sketched variant of the matrix least squares problem \eqref{equation:matrix_ls} with quantized response is given by,
\begin{equation}\label{equation:matrix_ls_ske_quand}
\widetilde{\rmB} = \argmin_{\bB \in \real^{p \times q}} \normf{\rmS\bX\bB - Q(\rmS\bY)}^2,
\end{equation}
where $\rmS \in \real^{m \times n}$ is a Gaussian sketching matrix with entries distributed as $\rs_{ij} \sim \normal(0, \frac{1}{m})$, and $Q \triangleq Q_{R, B}$ denotes the uniformly dithered quantizer defined in \eqref{equation:dither_quantizer_prob}. We assume that the dynamic range satisfies $R \geq \norm{\rmS\bY}_+$ so that the quantizer remains unsaturated. 
The solution to \eqref{equation:matrix_ls_ske_quand} can be written in closed form (Theorem~\ref{theorem:svd-deficient-rank}) as
\begin{equation}
	\widetilde{\rmB} = (\rmS\bX)^+ Q(\rmS\bY).
\end{equation}
Note that the norm $\norm{\cdot}_+$ denotes the maximum magnitude of the underlying matrix; see Problem~\ref{prob:maxmag_norm}.

The following theorem characterizes the accuracy of the approximate solution $\widetilde{\rmB}$ relative to the original least squares problem \eqref{equation:matrix_ls}.
\begin{theoremHigh}[Sketched LS with quantized response]\label{theorem:quantized_res_ls}
Let $\rmS \in \real^{m \times n}$ be a random Gaussian matrix with entries distributed as $\rs_{ij} \sim \normal(0, \frac{1}{m})$, and $Q \triangleq Q_{R, B}$ be a uniformly dithered quantizer with dynamic range $R$ and bit-budget $B$. Furthermore, suppose we are given matrices  $\bX \in \real^{n\times p}$ and $\bY \in \real^{n\times q}$, and define:
$$
\widehatbB \triangleq \argmin_{\bB \in \real^{p \times q}} \normf{\bX\bB - \bY}^2 
\qquad \text{and} \qquad 
\widetilde{\rmB} \triangleq \argmin_{\bB \in \real^{p \times q}} \normf{\rmS\bX\bB - Q(\rmS\bY)}^2.
$$
Let $\mathbf{E} \triangleq Q(\rmS\bY) - \rmS\bY$ be the quantization error matrix.
Then, if $R \geq \norm{\rmS\bY}_+$ and $\bS\bX$ has full column rank, we have
$$
\normf{\bX\widehatbB - \bY}^2 \leq \Exp \normf{\bX\widetilde{\rmB} - \bY}^2 \leq \frac{m-1}{m-r-1} \normf{\bX\widehatbB - \bY}^2 + \frac{q \Delta^2}{4} \frac{\sigma_{\max}^2}{\sigma_{\min}^2} \frac{m^2}{(n - m - 1)},
$$
where $r = \operatorname{rank}(\bX)$, and $\sigma_{\max}$ and $\sigma_{\min}$ denote the largest and smallest singular values of $\bX$, respectively.
\end{theoremHigh}
\begin{proof}[of Theorem~\ref{theorem:quantized_res_ls}]
The solution of the standard matrix least squares problem \eqref{equation:matrix_ls} can be expressed as:
$$
\widehatbB = [\widehatbbeta_1, \widehatbbeta_2, \ldots, \widehatbbeta_q], 
\quad 
\text{where} \quad \widehatbbeta_i = \argmin_{\bbeta \in \real^p} \normf{\bX\bbeta - \by_i}^2, \ \forall\, i,
$$
and $\by_i \in \real^n$ denote the $i$-th column of $\bY$. Therefore, we first analyze the vector case for the standard least squares problem:
\begin{equation}\label{equation:quant_ols}
\widehatbbeta = \argmin_{\bbeta \in \real^p} \normf{\bX\bbeta - \by}^2,
\end{equation}
and generalize the results by concatenating $\widehatbbeta_i$ to obtain $\widehatbB$. The sketched variant of \eqref{equation:quant_ols} with quantized response is given by,
\begin{equation}\label{equation:quant_sketcls}
\widetildebbeta = \argmin_{\bbeta \in \real^p} \normf{\rmS \bX\bbeta - Q(\rmS\by)}^2.
\end{equation}
The solution to \eqref{equation:quant_sketcls} is available in closed form as $\widetildebbeta = (\rmS\bX)^+ Q(\rmS\by)$ (Theorem~\ref{theorem:svd-deficient-rank}). 
Let $\be \triangleq Q(\rmS\by) - \rmS\by \in \real^m$   denote the quantization error. 
We then have,
\begin{equation}\label{equation:mls_quan_onecol}
\begin{aligned}
&\Exp \normtwobig{\bX\widetildebbeta - \by}^2 
= \Exp \normtwo{\bX (\rmS\bX)^+ Q(\rmS\by) - \by}^2 
= \Exp \normtwo{\bX (\rmS\bX)^+ (\rmS\by + \be) - \by}^2 \\
&=\Exp \normtwo{\bX (\rmS \bX)^+ \rmS \by - \by}^2 
+ \Exp \normtwo{\bX (\rmS \bX)^+ \be}^2 + \Exp \big[ \big(\bX (\rmS \bX)^+ \rmS \by - \by\big)^\top \bX (\rmS \bX)^+ \be \big] \\
&=\Exp \normtwo{\bX (\rmS \bX)^+ \rmS \by - \by}^2 
+ \Exp \normtwo{\bX (\rmS \bX)^+ \be}^2,
\end{aligned}
\end{equation}
where the last equality follows as the cross term disappears because,
\begin{equation}
\Exp \big[ \big(\bX (\rmS \bX)^+ \rmS \by - \by\big)^\top \bX (\rmS \bX)^+ \be \big] 
= 
\Exp_{\rmS} \big[ \big(\bX (\rmS \bX)^+ \rmS \by - \by\big)^\top \bX (\rmS \bX)^+ \be \big] 
= \bzero,
\end{equation}
where the last equality follows from  $\Exp_{Q} [\be] = \bzero$ when $Q$ is a uniformly dithered quantizer (Lemma~\ref{lemma:unif_dit_qua}).

Let  $\bE \triangleq \{e_{ij}\} = Q(\rmS\bY) - \rmS \bY \in \real^{m \times q}$  denote the quantization error matrix. 
Generalizing \eqref{equation:mls_quan_onecol} to the matrix least squares problem by treating each column $\by_i$ separately and summing over all columns yields:
\begin{align}
\Exp \normfbig{\bX \widetilde{\rmB} - \bY}^2 
&= \Exp \normf{\bX (\rmS \bX)^+ \rmS \bY - \bY}^2 
+ \Exp \normf{\bX (\rmS \bX)^+ \bE}^2\\
&=\frac{m-1}{m-r-1}\normf{\bX\widehatbB - \bY}^2
+ \Exp \normf{\bX (\rmS \bX)^+ \bE}^2, \label{equation:quanres_sktls_twoterm}
\end{align}
where the last equality follows from \eqref{equation:mls_lske_mainres2}.

\paragraph{Upper bounding $\Exp\normf{\bX (\rmS \bX)^+ \bE}^2$.} 
We now upper bound the second term in \eqref{equation:quanres_sktls_twoterm}. \eqref{equation:quanres_sktls_twoterm}. Since $\normf{ \bA}^2 = \trace(\bA^\top \bA )$ for any matrix $\bA$, using the cyclic invariance of trace, whence we have 
\begin{align}
\Exp \normf{\bX (\rmS \bX)^+ \bE}^2 &= \Exp \left[ \trace \left( \bE^\top \big((\rmS \bX)^+\big)^\top \bX^\top \bX (\rmS \bX)^+ \bE \right) \right]  \\
&= \Exp \left[ \trace \left( \big((\rmS \bX)^+\big)^\top \bX^\top \bX (\rmS \bX)^+ \bE\bE^\top \right) \right]  \\
&= \Exp_{\rmS} \left[ \trace \left( \big((\rmS \bX)^+\big)^\top \bX^\top \bX (\rmS \bX)^+ \Exp_Q \left[ \bE\bE^\top \right] \right) \right].\label{equation:quanres_sktls_twoterm_term2}
\end{align}
Since $R \geq\norm{\rmS \bY}_+$, from \eqref{equation:matrix_ls_ske} and Lemma~\ref{lemma:unif_dit_qua}, the $(i,j)$-th entry of the quantization error matrix $\bE$ satisfies
\begin{equation}
\Exp \left[ e_{ij} \right] = 0 
\qquad \text{and} \qquad 
\Var[e_{ij}] \leq \frac{\Delta^2}{4} = \frac{R^2}{(2^B - 1)^2},
\end{equation}
whence we have 
\begin{equation}
\Exp \big[ (\bE\bE^\top)_{ij} \big] = \sum_{k=1}^q \Exp \left[ e_{ik} e_{jk} \right] 
= 
\begin{cases}
q \Var[e_{ik}] \leq \frac{q \Delta^2}{4}, & \text{for } i = j; \\
0, & \text{for } i \neq j.
\end{cases}
\end{equation}
Therefore, the expectation $\Exp_Q \left[ \bE\bE^\top \right]$ is a diagonal matrix whose diagonal elements are upper bounded by $\frac{q \Delta^2}{4}$. 
Let  $\rmZ \triangleq \big((\rmS \bX)^+\big)^\top \bX^\top \bX (\rmS \bX)^+ \in\real^{m\times m}$, which is a random matrix depends on $\rmS$. Then, \eqref{equation:quanres_sktls_twoterm_term2} simplifies to,
\begin{equation}\label{equation:quanres_sktls_twoterm_term3}
\Exp_{\rmS} \left[ \trace \left( \rmZ \cdot \Exp_Q \big[  \bE\bE^\top \big] \right) \right] = \Exp_{\rmS} \left[ \sum_{i=1}^m \rz_{ii} \left( \Exp_Q \big[ \bE\bE^\top  \big] \right)_{ii} \right] \leq \frac{q \Delta^2}{4} \Exp_{\rmS} \left[ \trace(\rmZ) \right].
\end{equation}
For the $\Exp_{\rmS} \left[ \trace(\rmZ) \right]$ term, we have 
\begin{align}
\Exp_{\rmS} \left[ \trace(\rmZ) \right]
&= \Exp_{\rmS} \left[ \normf{\bX(\rmS \bX)^+}^2 \right] \stackrel{\dag}{\leq} \Exp_{\rmS} \left[ \normf{(\rmS \bX)^+}^2 \right] \sigma_{\max}^2(\bX)  \\
&= \Exp_{\rmS} \left[ \trace \left( (\rmS \bX \bX^\top \rmS^\top)^{-1} \right) \right] \sigma_{\max}^2(\bX)  \\
&\stackrel{\ddagger}{\leq} \frac{\sigma_{\max}^2(\bX)}{\sigma_{\min}^2(\bX)} \trace \left( \Exp \left[ (\rmS \rmS^\top)^{-1} \right] \right)
\stackrel{*}{=} \frac{\sigma_{\max}^2(\bX)}{\sigma_{\min}^2(\bX)} \frac{m^2}{(n - m - 1)},\label{equation:quanres_sktls_twoterm_term4}
\end{align}
where  the inequality ($\dag$) follows from the fact that $\normf{\bA\bB}\leq \normtwo{\bA}\normf{\bB}$~\footnote{See Problem~\ref{prob:frob_spec_bound}} for any matrices $\bA$ and $\bB$ with appropriate dimensions, the inequality ($\ddagger$) follows from the Loewner order for matrix products~\footnote{See Problem~\ref{prob:loewner_order}.}, and the equality ($*$) follows from Example~\ref{example:wis_invwish}. Combining \eqref{equation:quanres_sktls_twoterm_term2}, \eqref{equation:quanres_sktls_twoterm_term3} and \eqref{equation:quanres_sktls_twoterm_term4} yields,
\begin{equation}
\Exp \normf{\bX (\rmS \bX)^+ \bE }^2 \leq \frac{q \Delta^2}{4} \frac{\sigma_{\max}^2(\bX)}{\sigma_{\min}^2(\bX)} \frac{m^2}{(n - m - 1)}.
\end{equation}
This completes the proof. 
\end{proof}

Therefore, the theorem shows the residual matrix of the sketching matrix least squares problem with quantized response is bounded by the residual of the standard matrix least squares problem plus a term that depends on the condition number of the data matrix $\bX$.

\index{Low-rank approximation}
\index{Randomized SVD}
\index{Singular value decomposition}
\section{Least Squares Using Randomized SVD}

\index{Singular value decomposition}
In Section~\ref{section:ls-via-svd}, we introduced how to obtain the least squares solution using SVD.
However, for a large matrix $\bX\in\real^{n\times p}$, calculating its SVD rquires $\mathcalO(np^2)$ flops, which becomes prohibitively expensive as the size of the data matrix increases.
To address this issue, a randomized algorithm samples some columns from $\bX$ to construct a smaller matrix $\bC \in \real^{n \times m}$ ($m<p$), aiming to approximate $\bX$ with these sampled columns, expressed as $\bC = \bX \bS$.  The goal is to minimize:
$$
\min_{\bB} \normf{\bC \bB - \bX}^2 = \min_{\bB} \normf{\bX \bS \bB - \bX}^2.
$$
This is column-wise decomposable:
$$
\begin{aligned}
&\arg\min_{\bbeta_i} \sum_{i=1}^p \normtwo{\bX \bS \bbeta_i - \bx_i}^2 = (\bX \bS)^+ \bx_i, \quad \forall\, i \in\{1,2,\ldots,p\}\\
&\implies 
\argmin_{\bB} \normf{\bX \bS \bB - \bX}^2 = (\bX \bS)^+ \bX.
\end{aligned}
$$
The matrix $\bX$ is then approximated by $\bX\approx \bX\bS\widehatbB = (\bX \bS)(\bX \bS)^+ \bX = \bC\bC^+ \bX$.
Given that $\bC\in\real^{n\times m}$ with $m<p$, the truncated SVD (TSVD) of $\bC\approx\bU_k\bSigma_k\bV_k^\top$---such that $\bC\bC^+\approx \bU_k \bU_k^\top$---costs $\mathcalO(nm^2)$, and the low-rank approximation $\underbrace{\bU_k}_{n\times k}\underbrace{\bU_k^\top\bX}_{k\times p}$ of $\bX$ thus costs less than the optimal rank-$k$ approximation of $\bX$ using the SVD of $\bX$ ($\mathcalO(np^2)$ flops).

The least squares using the rank-$k$ approximation $\bU_k \bU_k^\top \bX$ of $\bX$ then reduces the computational complexity and aligns  with the goal of the norm ratio methods introduced in Section~\ref{section:svd_nmratio_method}.
The following theorem establishes an upper bound for the error introduced by this low-rank approximation:
\begin{theoremHigh}[Randomized Low-Rank Approximation \citep{drineas2006fast}]\label{theorem:spe_err_ran_spe}
Let $\bX\in\real^{n\times p}$, and let $\bC\in\real^{n\times m}$ contain $m$ columns of $\bX$ with $m<p$. 
Suppose further that the rank-$k$ TSVD of $\bC$ is $\bC\approx\bU_k\bSigma_k\bV_k^\top$ ($k\leq m$). Then,
\begin{subequations}\label{equation:spe_err_ran_spe}
\begin{align}
\normtwo{\bX - \bU_k \bU_k^\top \bX}^2 &\leq \normtwo{\bX - \bX_k}^2 + 2 \normtwo{\bX\bX^\top - \bC\bC^\top}; \label{equation:spe_err_ran_spe1}\\
\normf{\bX - \bU_k \bU_k^\top \bX}^2 &\leq \normf{\bX - \bX_k}^2 + 2 \sqrt{k} \normf{\bX \bX^\top - \bC \bC^\top}, \label{equation:spe_err_ran_spe2}
\end{align}
\end{subequations}
where $\bX_k$ denotes the optimal rank-$k$ approximation of $\bX$ by TSVD (Theorem~\ref{theorem:young-theorem_frob}).
\end{theoremHigh}
Equation~\eqref{equation:spe_err_ran_spe} shows that the low-rank approximation by $\bU_k\bU_k^\top\bX$ is bounded by two terms, where the first term is the approximation error of the exact SVD, and the second term is the spectral norm (or Frobenius norm)  error in the matrix multiplication approximation.
\begin{proof}[of Theorem~\ref{theorem:spe_err_ran_spe}]
For brevity, we only prove \eqref{equation:spe_err_ran_spe1} here; \eqref{equation:spe_err_ran_spe2} follows similarly and is left as an exercise. 
It follows that 
$$
\small
\begin{aligned}
\normtwo{\bX - \bU_k \bU_k^\top \bX}
&\stackrel{\dag}{=} \max_{\normtwo{\ba} = 1} \normtwo{\ba^\top (\bX - \bU_k \bU_k^\top \bX)} 
=  \max_{\substack{\normtwo{\bb} = \normtwo{\bc} = 1, \\
\bb \in \cspace(\bU_k), \bc \in \cspace(\bU_k)^\perp, \\ 
\alpha^2 + \beta^2 = 1}} \normtwo{(\alpha \bb + \beta \bc)^\top (\bX - \bU_k \bU_k^\top \bX)} \\
&\stackrel{\ddag}{\leq} \max_{\normtwo{\bc} = 1, \bc \in \cspace(\bU_k)^\perp} \normtwo{\bc^\top (\bX - \bU_k \bU_k^\top \bX)} +
\max_{\normtwo{\bb} = 1, \bb \in \cspace(\bU_k)} \normtwo{\bb^\top (\bX - \bU_k \bU_k^\top \bX)} \\
&= \max_{\normtwo{\bc} = 1, \bc \in \cspace(\bU_k)^\perp} \normtwo{\bc^\top (\bX - \bU_k \bU_k^\top \bX)}
= \max_{\normtwo{\bc} = 1, \bc \in \cspace(\bU_k)^\perp} \normtwo{\bc^\top \bX},
\end{aligned}
$$
where the equality ($\dag$) follows from the definition of the spectral norm (Definition~\ref{definition:spectral_norm}), the inequality ($\ddag$) follows because $\alpha, \beta\leq 1$ and $\bb^\top\bU_k\bU_k^\top=\bzero$ can be proved by letting $\bb\triangleq\bU_k\bv$ with $\bv\neq\bzero$.
Taking squares
$$
\small
\begin{aligned}
& \max_{\normtwo{\bc} = 1, \bc \in \cspace(\bU_k)^\perp} \normtwo{\bc^\top \bX}^2 
= \max_{\normtwo{\bc} = 1, \bc \in \cspace(\bU_k)^\perp} \bc^\top \bC \bC^\top \bc + \bc^\top (\bX \bX^\top - \bC \bC^\top) \bc \\
&\stackrel{\dag}{\leq} \max_{\normtwo{\bc} = 1, \bc \in \cspace(\bU_k)^\perp} \sigma_{k+1}^2(\bC) + \normtwo{\bX \bX^\top - \bC \bC^\top}
\stackrel{\ddag}{\leq} \max_{\normtwo{\bc} = 1, \bc \in \cspace(\bU_k)^\perp} \sigma_{k+1}^2(\bX) + 2 \normtwo{\bX \bX^\top - \bC \bC^\top},
\end{aligned}
$$
where the inequality ($\dag$) follows from Problem~\ref{prob:raylei_ritz} ($\bC\bC^\top$ and $\bX\bX^\top-\bC\bC^\top$ are symmetric such that $\sigma_{k+1}^2(\bC) \equiv  \sigma_{k+1}(\bC\bC^\top)\equiv\lambda_{k+1}(\bC\bC^\top)$), and the inequality ($\ddag$) follows from the matrix perturbation result
$
\sigma_{k+1}(\bC \bC^\top) - \sigma_{k+1}(\bX \bX^\top)
\leq \normtwo{\bX \bX^\top - \bC \bC^\top}
$ by Problem~\ref{prob:weyl_theo}.
This completes the proof.
\end{proof}

\paragrapharrow{Gaussian sketching.}
Another interesting problem is the low-rank reconstruction of $\bX$ using the sampled columns in $\bC=\bX\bS$:
$$
\bB^* = \argmin_{\bB} \normf{\bC\bB-\bX}^2
=(\bX\bS)^+ \bX.
$$
Let $\bX_k$ denote the optimal rank-$k$ approximation of $\bX$ (Theorem~\ref{theorem:young-theorem_frob}).
It then follows that 
$$
\begin{aligned}
\normf{(\bX\bS)(\bX\bS)^+ \bX-\bX}^2
&\leq  \normf{(\bX\bS)(\bX_k\bS)^+ \bX_k-\bX}^2\\
&=\normf{\bX_k^\top (\bS^\top\bX_k^\top)^+ (\bS^\top\bX^\top)  -\bX^\top}^2
&\triangleq \normf{\bX_k^\top \widetildebB  -\bX^\top}^2
\end{aligned}
$$
where $\widetildebB \triangleq (\bS^\top\bX_k^\top)^+ (\bS^\top\bX^\top) = \argmin_{\bB} \normf{\bS^\top\bX_k^\top\bB - \bS^\top\bX^\top}^2$ by Theorem~\ref{theorem:svd-deficient-rank}.
This optimization problem is known as the sketched problem of $\argmin_{\bB} \normf{\bX_k^\top\bB - \bX^\top}^2$, where the sketching is performed by $\bS^\top$.

When treating $\bS$ as a Gaussian sketching matrix, \eqref{equation:mls_lske_mainres2} provides the Frobenius norm error bound:
\begin{equation}
\Exp \normf{\bX\rmS(\bX\rmS)^+ \bX - \bX }^2 
\leq \Exp \normf{\bX_k^\top\widetilde{\rmB} - \bX^\top }^2 
\leq \frac{m-1}{m-k-1} \normf{\bX_k - \bX}^2,
\end{equation}
which is valid for any $k \in \{1,2, \ldots, \rank(\bX)\}$.

\index{Singular value decomposition}
\paragrapharrow{Randomized SVD.}
The randomized SVD approach is closely related to the randomized low-rank approximation method introduced here.
Since $\bC\bC^+ \bX = (\bX\bS)(\bX\bS)^+ \bX \approx \bX$, we can calculate the reduced QR decomposition of $\bX\bS = \bQ\bR$ ($\bQ\in\real^{n\times m}$ and $\bR\in\real^{m\times m}$, which costs $\mathcal(nm^2)$; see Section~\ref{section:ls_qr_gen}), whence we have 
$$
\bC\bC^+\bX =\bQ\bQ^\top \bX \approx \bX,
$$
i.e., $\bQ$ approximates the column space of $\bX$.
Calculating the SVD $\bQ^\top \bX = \bU \bSigma \bV^\top\in\real^{m\times p}$ (which costs $\mathcal(mp^2)$),  the approximate  SVD of $\bX$ is then given by $\bX \approx (\bQ\bU) \bSigma \bV^\top$.
For finding the rank-$k$ low-rank approximation, 
\citet{halko2011finding} shows that
$$
\Exp \normtwo{\bX - \bQ\bQ^\top \bX} \leq \left(1 + \frac{4\sqrt{m}}{m-k-1} \sqrt{\min(n,p)}\right) \sigma_{k+1}.
$$
Thus, the expected approximation error is upper bounded by this expression for any target rank $k$.

\begin{problemset}
\item Prove \eqref{equation:mean_gauslef_i}. That is, let $\rmS\in\real^{m\times n}$ be a random matrix whose entries are independent random variables distributed as $\rs_{ij}\sim \frac{1}{\sqrt{m}}\normal(0,1)$. Show that $\Exp[\rmS^\top \rmS] = \bI $.
\item \label{prob:frob_spec_bound} Let $\bA\in\real^{m\times n}$ and $\bB\in\real^{n\times p}$. Show that $\normf{\bA\bB}\leq \normtwo{\bA}\normf{\bB}$, where $\normf{\cdot}$  denotes the Frobenius norm and  $\normtwo{\cdot}$ denote the spectral norm of a matrix. \textit{Hint: use the definition of Frobenius and spectral norms; Definitions~\ref{definition:frobernius-in-svd} and \ref{definition:spectral_norm}.}

\item \label{prob:loewner_order} Let $\bA\in\real^{m\times n}$ and $\bB\in\real^{n\times p}$. Show that 
$$
\sigma_{\min}^2(\bA) \bB^\top \bB \preceq \bB^\top \bA^\top \bA \bB \preceq \sigma_{\max}^2(\bA) \bB^\top \bB.
$$
\textit{Hint: Prove the first result in the order $
\bx^\top \bB^\top \bA^\top \bA \bB \bx = \normtwo{\bA \bB \bx }^2 
\stackrel{*}{\geq}
\sigma_{\min}^2(\bA) \normtwo{\bB \bx}^2 = \bx^\top \big(\sigma_{\min}^2(\bA) \bB^\top \bB\big) \bx
$, and prove the inequality ($*$) using SVD.}

\item \label{prob:maxmag_norm} Show that the function $\norm{\cdot}_+:\real^{n\times p}\rightarrow \real$, which returns the maximum absolute value of any entry in the matrix, defines a valid matrix norm satisfying all the conditions of a matrix norm introduced in Definition~\ref{definition:matrix-norm}. Is this norm submultiplicative?

\index{Rayleigh-Ritz theorem}
\item \label{prob:raylei_ritz} \textbf{Rayleigh-Ritz theorem.} Let $\bA\in\real^{n\times n}$ be symmetric with the spectral decomposition $\bA=\bQ\bLambda\bQ^\top$.
Show that if $\mathcalV$ is the subspace spanned by  $\{\bq_p, \bq_{p+1}, \ldots, \bq_q\}$, then 
\begin{equation}\label{equation:rr_rr_2}
\begin{aligned}
	&\mathop{\max}_{\bx\neq \bzero, \bx\in\mathcalV}
	\frac{\bx^\top\bA\bx}{\bx^\top\bx}
	=
	\lambda_q
	\qquad\text{and}\qquad
	\mathop{\min}_{\bx\neq \bzero, \bx\in\mathcalV}
	\frac{\bx^\top\bA\bx}{\bx^\top\bx}
	=
	\lambda_p,
\end{aligned}
\end{equation}
where $\lambda_1\leq \lambda_2\leq \ldots\leq \lambda_n$ denote the eigenvalues of $\bA$.

\index{Weyl's theorem}
\item \label{prob:weyl_theo} \textbf{Weyl's theorem.} Let $\bA\in\real^{n\times n}$ be symmetric. Show that 
\begin{equation}\label{equation:lam_ineq_pert_max}
	\abs{\lambda_k(\bA) - \lambda_k(\bA+\bE)} \leq \normtwo{\bE},\gap k\in\{1,2,\ldots, n\},\ \forall\, \bE,
\end{equation}
where $\lambda_k(\bA)$ denotes the $k$-th eigenvalue of $\bA$.
That is, the eigenvalues of a real symmetric matrix  are {stable under  small perturbations}.

\end{problemset}

\newpage
\chapter{The Bayesian Approach}\label{chapter:bayes_app_mle}
\begingroup
\hypersetup{
linkcolor=structurecolor,
linktoc=page,  % page: only the page will be colored; section, all, none etc
}
\minitoc \newpage
\endgroup
\section{The Bayesian Approach}\label{sec:bayesian-approach}
\lettrine{\color{caligraphcolor}I}
In contemporary statistics, Bayesian approaches have become increasingly significant and widely utilized. 
Thomas Bayes came up with this idea but died before publishing it. Fortunately, his friend Richard Price carried on his work and published the work in 1764. 
In this section, we outline the foundational principles of Bayesian methodology, using the \textit{Beta-Bernoulli model} as an introductory example to highlight the benefits of Bayesian models. Additionally, we explore powerful Bayesian techniques for linear models and their relationship to ordinary least squares.

Note that we have previously introduced the Bayesian estimation method in Section~\ref{section:bayes_esti}, from which the Laplace approximation method was derived (Section~\ref{section:laplace_approx}). The Laplace approximation method further enables the definition of the useful BIC criterion for model selection; see Section~\ref{section:bic}. Here, we briefly compare frequentist and Bayesian approaches and introduce the concept of conjugate priors.

\index{Bayes' theorem}
\index{Bayes' rule}
The core  idea of the Bayesian approach, in a nutshell, involves assuming a \textit{prior} probability distribution over the unknown parameter $\btheta$ with hyper-parameters $\balpha$ (i.e., $p(\btheta)=p(\btheta\mid  \balpha)$)---a distribution representing the plausibility of each possible value of $\btheta$ before observing the data. 
Consequently, to infer information about $\btheta$, one considers the conditional distribution of $\btheta$  given the observed data, known as the posterior distribution. The posterior reflects the plausibility of each potential value of $\btheta$  after taking the data into account. We then apply probability rules to address specific questions of interest \citep{fahrmeir2007regression, hoff2009first}. For instance, when determining the parameter based on the maximum posterior probability of $\btheta$, we refer to the maximum a posteriori (MAP) estimator; see Definition~\ref{definition:map_esti}.

\subsection*{Comparison of Frequentist and Bayesian Approaches}
In the field of statistical inference, there are traditionally two distinct schools of thought: the \textit{frequentist approach} and the \textit{Bayesian approach}. While we will not explore their underlying philosophical differences in this discussion, we will focus on how they differ in their treatment of the parameter space.

Typically, we use a probability distribution $p(\mathcalX \mid \btheta)$ to describe a random variable. For example, we might assume that a random variable $\rx$ follows a Bernoulli distribution. In general, any probability distribution includes one or more parameters, and only when these parameters are specified can the distribution be fully determined.

When the values of these parameters are unknown, our goal is to estimate them so that the distribution can be properly defined. Once estimated, the distribution becomes useful for making meaningful inferences or predictions. The frequentist and Bayesian approaches differ fundamentally in how they interpret and handle such unknown parameters.

\paragrapharrow{Frequentist approaches.}
The frequentist approach assumes that the parameters of a probability distribution are fixed but unknown constants---essentially just numerical values. 
These parameters define the distribution $p(\mathcalX \mid \btheta)$ through parameterization, and frequentists assume there exists a single true value of $\btheta$ within the parameter space. 
The objective is to estimate this optimal (true) value. To do so, frequentists often use methods such as maximum likelihood estimation (see Section~\ref{section:mle_method}), which seeks the value of $\btheta$ that makes the observed data most probable.

To estimate the parameter  $\btheta$, we require a set of observed samples from the random variable $\mathcalX$. These samples are assumed to come from the same underlying distribution $ p(\mathcalX \mid \btheta) $, meaning they are identically distributed. Furthermore, in many cases, the samples are assumed to be independent and identically distributed.

We denote the collection of observed samples as $ \mathcalX = \{\bx_1, \bx_2, \ldots, \bx_n\} $. The probability of observing each individual sample is given by $ p(\bx_i\mid \btheta) $. Assuming independence, the joint probability of observing all samples together is:
$$
p(\mathcalX\mid \btheta) = p(\bx_1, \bx_2, \ldots, \bx_n \mid  \btheta) = \prod_{i=1}^{n} p(\bx_i\mid  \btheta),
$$
where  $ p(\mathcalX\mid \btheta) $ is commonly referred to as the likelihood function. 
This function is conventionally denoted by  $ \mathcalL(\btheta; \mathcalX) $. 
The idea behind maximum likelihood estimation is to find the value of $\btheta$ that maximizes the likelihood of observing the given dataset.
Therefore, the maximum likelihood estimate is defined as:
$$ 
\widehat{\btheta}_{\ML} = \arg\max_{\btheta} p(\mathcalX\mid \btheta) = \arg\max_{\btheta} \mathcalL(\btheta; \mathcalX).
$$
Once we have estimated $\btheta$, the corresponding probability distribution becomes $ p(\mathcalX\mid  \widehat{\btheta}_{\ML}) $, and we can use it to make predictions about new observations:
$$ 
p(\mathcalX = \bx_{\new}) = p(\mathcalX = \bx_{\new}\mid \widehat{\btheta}_{\ML}) .
$$

\paragrapharrow{Bayesian approaches.}
However, the Bayesian approach takes a fundamentally different perspective.
In Bayesian statistics, we recognize that since the parameter $\btheta$ is unknown, any value within the parameter space is, at least initially, possible. Therefore, rather than treating $\btheta$ as a fixed constant, it is considered a random variable---a central tenet of the Bayesian philosophy is that all unknown quantities should be treated probabilistically.
This leads to two key ideas:
\begin{itemize}
\item  The probability distribution over the data $\mathcalX$ is defined conditionally on the unknown parameter $\btheta$, written as $p(\mathcalX\mid \btheta)$. This formulation reflects a crucial difference from the frequentist perspective.
\item Using Bayes' theorem (Theorem~\ref{theorem:baye_theo_mle}), we reverse the relationship between the observed data $\mathcalX$ and the parameter $\btheta$, allowing us to infer the distribution of $\btheta$ given the data:
\begin{equation}\label{equation:ost_thi_x_2}
\begin{aligned}
p(\btheta \mid  \mathcalX) 
= \frac{p(\mathcalX \mid  \btheta ) p(\btheta )}{p(\mathcalX )} 
= \frac{p(\mathcalX \mid  \btheta ) p(\btheta )}{\int_{\btheta}  p(\mathcalX, \btheta ) }  
= \frac{p(\mathcalX \mid  \btheta ) p(\btheta )}{\int_{\btheta}  p(\mathcalX \mid  \btheta ) p(\btheta ) } 
\propto p(\mathcalX \mid  \btheta ) p(\btheta ),
\end{aligned}
\end{equation}
where $\mathcalX$ represents the observed dataset or random variable, and the notation $f(x)\propto g(x)$ means that $f$ is proportional to $g$ up to a normalizing constant. 
\end{itemize}
Through Bayes' theorem, we formally define the relationship between the random variable $ \mathcalX $ and the parameter variable $ \btheta $. The term $ p(\btheta) $ in the formula represents the marginal probability distribution of the parameter variable $ \btheta $, which is the probability distribution of $ \btheta $ before observing any data. 
Therefore, we call it the \textit{prior distribution} of $ \btheta $. 

In practice, we typically do not know the true distribution of $\btheta$, so we assume a known form for the prior. 
Often, this prior depends on one or more hyper-parameters $\balpha$, in which case we write the prior as $p(\btheta\mid \balpha)$. Substituting this into Equation~\eqref{equation:ost_thi_x_2}, we obtain:
\begin{equation}\label{equation:ost_thi_x_3}
\begin{aligned}
p(\btheta \mid  \mathcalX, \balpha) 
&= \frac{p(\mathcalX \mid  \btheta ) p(\btheta \mid  \balpha)}{p(\mathcalX \mid  \balpha)} = \frac{p(\mathcalX \mid  \btheta ) p(\btheta \mid  \balpha)}{\int_{\btheta}  p(\mathcalX, \btheta \mid  \balpha) }   \\
&= \frac{p(\mathcalX \mid  \btheta ) p(\btheta \mid  \balpha)}{\int_{\btheta}  p(\mathcalX \mid  \btheta ) p(\btheta \mid  \balpha) } 
\propto p(\mathcalX \mid  \btheta ) p(\btheta \mid  \balpha).
\end{aligned}
\end{equation}
In the remainder of this discussion, we will omit the hyper-parameter $\balpha$ for simplicity.

Given the prior distribution $ p(\btheta) $ and the conditional probability distribution $ p(\cdot \mid \btheta) $, we observe a set of samples from the random variable $ \mathcalX $, denoted as $ \mathcalX = \{\bx_1, \bx_2, \ldots, \bx_n\} $. 
Unlike frequentist approaches, these samples are assumed to be drawn independently from the joint distribution: $ p(\mathcalX, \btheta) = p(\mathcalX\mid \btheta)p(\btheta) $; as opposed to be drawn from the conditional distribution $p(\mathcalX\mid \btheta)$ in frequentist approaches. 

Our goal is to infer the true distribution of the parameter $ \btheta $ given the observed data $\mathcalX$. 
That is, we aim to compute the posterior distribution $ p(\btheta \mid  \mathcalX)  $, which by \eqref{equation:ost_thi_x_2} is
$$
p(\btheta \mid  \mathcalX) 
= \frac{p(\mathcalX \mid  \btheta ) p(\btheta )}{p(\mathcalX )} .
$$
The conditional probability distribution $ p(\btheta\mid \mathcalX) $ is called the \textit{posterior distribution} of the parameter $ \btheta $, because it is the probability distribution of $ \btheta $ under the condition of observing the samples. 
This posterior distribution represents our updated belief about $ \btheta $ after observing the sample set. It serves as our estimate of the parameter $ \btheta $, incorporating both the prior knowledge and the information provided by the data.

Alternatively, the posterior can be interpreted as being proportional to the product of the likelihood and the prior:
\begin{equation}
\mathrm{Posterior} = \frac{\mathrm{Likelihood}\times \mathrm{Prior}}{\mathrm{Marginal\,\, likelihood}}  \propto \mathrm{Likelihood} \times \mathrm{Prior}, 
\end{equation}
where the proportionality constant is determined by the marginal likelihood (or evidence), $p(\mathcalX)$. This formulation allows us to quantify uncertainty in the parameter estimates.

\paragrapharrow{Predictive inference.}
For a new observation  $\bx_{\new}$ maximum likelihood estimation  predicts using the likelihood evaluated at the MLE estimate:
$p(\mathcalX = \bx_{\new}\mid \widehat{\btheta}_{\ML})$.
In contrast, Bayesian inference uses the posterior distribution $p(\btheta \mid  \mathcalX) $
to compute the predictive distribution for new samples. Assuming the generative process $ p(\mathcalX, \btheta) = p(\btheta)p(\mathcalX\mid \btheta) $, the predictive distribution becomes:
$$ 
p(\mathcalX = \bx_{\new}) = \int  p(\mathcalX = \bx_{\new}\mid \btheta)p(\btheta\mid \mathcalX)  d\btheta. 
$$
If the problem  follows from a generative process $\rvy\sim p(\by\mid \bx,\btheta)$, e.g., $\ry\sim \bbeta^\top\bx_{\new} +\epsilon$ in the Gauss-Markov linear model. 
Then the predictive distribution is 
$$
p(\by^\prime\mid\bx_{\new},\mathcalX, \mathcalY)=\int p(\by^\prime \mid \bx_{\new}, \btheta)p(\btheta \mid \mathcalX,\mathcalY)d\btheta.
$$

\paragrapharrow{Point estimates from the posterior.}
Instead of using the full posterior distribution, one may summarize it using point estimates. Common choices include:
\begin{itemize}
\item \textit{The posterior mean.} Often  referred to as the \textit{Bayesian estimate}; that is, using the expected value $ \widehat{\btheta}_{\text{mean}} $ of $ \btheta $ from the posterior distribution:
\begin{equation}
\begin{aligned}
\widehat{\btheta}_{\text{Bayes}}\triangleq 
&\widehat{\btheta}_{\text{mean}} = \Exp_{p(\btheta\mid \mathcalX)}[\btheta] = \int \btheta p(\btheta\mid \mathcalX) d\btheta 
\quad \implies \quad
%p(\mathcalX = \bx_{\new}) 
p(\mathcalX = \bx_{\new}\mid \btheta = \widehat{\btheta}_{\text{mean}}) .
\end{aligned}
\end{equation}

\item \textit{The posterior median.} As one can expect, the median value  $ \widehat{\btheta}_{\text{median}} $ of $ \btheta $ from the posterior distribution $p(\btheta\mid \mathcalX)$ can also be regarded as a point estimate:
\begin{equation}
	\widehat{\btheta}_{\text{median}} = \text{ median of } p(\btheta\mid \mathcalX)
	\qquad \implies \qquad
	p(\mathcalX = \bx_{\new}\mid \btheta = \widehat{\btheta}_{\text{median}}) .
\end{equation}
\item \textit{The posterior maximum.} See the following paragraphs. 
\end{itemize}

Computing these point estimates requires knowing the exact form of the posterior distribution $p(\btheta\mid \mathcalX)$.
However, the denominator in Bayes' rule---the marginal likelihood or evidence in \eqref{equation:ost_thi_x_2}---can be difficult to compute:
$$
p(\mathcalX) = \int p(\mathcalX\mid \btheta) p(\btheta) d\btheta.
$$
This integral spans the entire parameter space and is often analytically intractable or computationally expensive.

To address this issue, two common strategies are used:
\begin{itemize}
\item Using conjugate priors:  If the prior and likelihood belong to conjugate families, the posterior has the same functional form as the prior, simplifying computation.
\item Using maximum a posterior (MAP) estimate: Instead of computing the full posterior, MAP finds the mode of the posterior distribution and avoids computing the evidence.
\end{itemize}

\paragrapharrow{Conjugate prior.}
Formally, we define the conjugate prior as follows:
\begin{definition}[Conjugate prior]\label{definition:conjug_prio}
In Bayesian inference, if the posterior distribution and the prior distribution belong to the same probability distribution family, then this prior distribution is called a \textit{conjugate prior}. Note that since the posterior distribution is obtained by multiplying the prior with the likelihood, conjugacy refers to the fact that the prior and the likelihood are conjugate. The conjugate prior, when multiplied by the likelihood, does not change the functional form of the distribution, so the posterior has the same form as the prior.
\end{definition}

Using a conjugate prior ensures that the posterior distribution has the same form as the prior, which often allows us to derive the posterior directly without having to compute the marginal likelihood $ p(\mathcalX )$ in \eqref{equation:ost_thi_x_2}. This significantly reduces the computational complexity involved in determining the posterior distribution.

For example, the conjugate prior for the likelihood function of a Gaussian distribution is itself a Gaussian distribution. Similarly, the conjugate prior for the Bernoulli likelihood is the Beta distribution, and for the categorical likelihood, it is the Dirichlet distribution. For further examples, see \citet{hoff2009first, lu2022bayesian}.

Despite their advantages, conjugate priors do have limitations. First, only members of the exponential family of distributions have conjugate priors. Second, choosing a conjugate prior is often motivated more by computational convenience than by the desire to achieve more accurate parameter estimation.

\paragrapharrow{MAP estimation.}
As mentioned above, in Bayesian inference, we often cannot compute the expectation of the posterior probability distribution directly. Even if we derive the exact form of the posterior distribution, calculating its expectation typically involves integration, which is analytically or computationally difficult in many cases. Therefore, an alternative point estimation method known as maximum a posteriori (MAP) estimation is commonly used; see Definition~\ref{definition:map_esti}:
$$
\widehat{\btheta}_{\MAP} = \arg\max_{\btheta} p(\btheta \mid  \mathcalX).
$$

The idea behind MAP estimation is to use the value of $\btheta$ that maximizes the posterior probability as our estimate, rather than using the expectation or median of the posterior distribution. From Equation~\eqref{equation:ost_thi_x_2}, we know that the posterior is proportional to the product of the likelihood and the prior:
\begin{equation}\label{equation:map_post}
p(\btheta \mid  \mathcalX) \propto p(\mathcalX \mid  \btheta) p(\btheta) \equiv \text{Likelihood} \times \text{Prior}.
\end{equation}
To obtain the MAP estimate, we only need to maximize the numerator of the posterior (that is, the product of the likelihood and the prior), so there is no need to compute the marginal likelihood $p(\mathcalX)$. This simplifies the computation significantly:
$$
\widehat{\btheta}_{\MAP} = \arg\max_{\btheta} p(\btheta \mid  \mathcalX)
\equiv \arg\max_{\btheta} p(\mathcalX \mid  \btheta) p(\btheta).
$$
Once we have obtained the MAP estimate $\widehat{\btheta}_{\MAP}$, we can use it as a point estimate for the unknown parameter $\btheta$. This allows us to make predictions for new data points:
$$
p(\mathcalX = \bx_{\new}) = p(\mathcalX = \bx_{\new} \mid  \btheta = \widehat{\btheta}_{\MAP}).
$$

Furthermore, MAP estimation can be interpreted as a modified version of maximum likelihood estimation, where a prior distribution acts as a regularizer or constraint:
$$
\widehat{\btheta}_{\MAP} = \arg\max_{\btheta} p(\mathcalX \mid  \btheta) p(\btheta)
= \arg\max_{\btheta} \{\ln p(\mathcalX \mid  \btheta) + \ln p(\btheta)\}.
$$
If the prior is flat---meaning $ \ln p(\btheta)=0$, for instance, when $p(\btheta)$ is a uniform distribution over the entire parameter space---then the MAP estimate reduces exactly to the MLE.

\paragrapharrow{Frequentists V.S. Bayesian in a nutshell.}
The \textit{frequentist approach} to statistics  evaluates statistical procedures based on a probability distribution over all possible data sets.
To be more specific, frequentists consider the parameter vector $\btheta$ to be fixed (albeit unknown), while introducing uncertainty over possible data sets $\mathcalX$. Frequentist methods are often considered more objective as they avoid incorporating subjective prior information. 
In contrast, Bayesian methods allow for the incorporation of prior beliefs. The Bayesian approach treats the data set $\mathcalX$ as given, while introducing uncertainty over $\btheta$. 
Moreover, though we will not use any hierarchical models in this book, Bayesian modeling is often more flexible, allowing for the specification of complex hierarchical models. 
This flexibility is advantageous in cases where the underlying data-generating process is intricate.
However, statisticians nowadays tend to move comfortably between these approaches and popular statistical procedures often combine both of them,  incorporating Bayesian methods for certain aspects of the analysis while using frequentist methods for others. 
For instance, empirical Bayesian methods  have a Bayesian spirit but are not strictly Bayesian; their analysis is frequently frequentist \citep{haugh2021tutorial}.

\index{Beta-Bernoulli model}
\section{An Appetizer: Beta-Bernoulli Model}\label{sec:beta-bernoulli}
We now formally introduce the Beta-Bernoulli model to illustrate how the Bayesian approach works.
The Bernoulli distribution models binary outcomes---that is, it assigns probabilities to two possible values, typically denoted as 0 and 1. The likelihood under this model is defined by the probability mass function of the Bernoulli distribution:
\begin{equation}
\bernoulli(x\mid \theta) = p(x\mid \theta) = \theta^x (1-\theta)^{1-x} \indicator(x\in \{0,1\}). \nonumber
\end{equation}
This means that:
$$
\bernoulli(x\mid \theta)=p(x\mid \theta)=\left\{
\begin{aligned}
&1-\theta ,& \mathrm{\,\,if\,\,} x = 0;  \\
&\theta , &\mathrm{\,\,if\,\,} x =1,
\end{aligned}
\right.
$$
where $\theta$ represents the probability of observing the outcome 1, and $1-\theta$ is the probability of observing 0.
The mean (or expected value) of the Bernoulli distribution is simply $\theta$. 
Suppose we are given a dataset $\mathcalX=\{x_1, x_2, ..., x_n\}$, where each $x_i$ is i.i.d. according to $\bernoulli(x\mid \theta)$. Then, the likelihood under the Bernoulli distribution is given by:
$$
\begin{aligned}
\text{Likelihood} = 	p(\mathcalX \mid \theta) &= \theta^{\sum x_i} (1-\theta)^{n-\sum x_i},
\end{aligned}
$$
which is a distribution on $\mathcalX$ and is called the \textit{likelihood function} on $\mathcalX$.

In this model, the prior distribution follows the probability density function of the \textit{Beta distribution}, which is defined as:
\begin{equation}
\mathrm{Prior} = \betadist(\theta\mid a, b)=	p(\theta\mid a, b) =\frac{1}{B(a,b)} \theta^{a-1}(1-\theta)^{b-1} \indicator(0\leq \theta \leq 1), \nonumber
\end{equation}
where $B(a,b)$ denotes   \textit{Euler's beta function}, serving as a normalization constant.  
And $\indicator(a\leq x\leq b)$ is a step function that has a value of 1 when $a\leq x\leq b$  and 0 otherwise (when $x<a$ or $a>b$). Figure~\ref{fig:dists_beta}  compares different parameters for the Beta distribution. 
Specifically, when $a=b=1$, the Beta distribution reduces to a \textit{uniform distribution} over the support of $[0,1]$.

\index{Beta distribution}
We place a Beta prior over the parameter $\theta$ of the Bernoulli distribution. The posterior distribution is then obtained as follows:
\begin{equation}
\begin{aligned}
\mathrm{Posterior} = p(\theta\mid \mathcalX) &\propto p(\mathcalX \mid \theta) p(\theta\mid a,b) \\
&=\theta^{\sum x_i} (1-\theta)^{n-\sum x_i} \times \frac{1}{B(a,b)} \theta^{a-1}(1-\theta)^{b-1}\indicator(0<\theta<1) \\
&\propto \theta^{a+\sum x_i-1}(1-\theta)^{b+n-\sum x_i-1} \indicator(0<\theta<1) \\
&\propto \betadist\left(\theta \mid  a+\sum x_i, b+n-\sum x_i\right). \nonumber
\end{aligned}
\end{equation}

\index{Conjugate prior}
We observe that the posterior distribution has the same functional form as the prior. When this occurs, we refer to the prior as a  \textit{conjugate prior} (Definition~\ref{definition:conjug_prio}). A conjugate prior is particularly useful because it simplifies computation: it allows for straightforward derivation of the posterior probability density function, its derivatives, and even sampling from the posterior.

Using conjugate priors has a key advantage: it preserves the mathematical form of the prior during Bayesian updating. As a result, the posterior can often be expressed in closed form, eliminating the need for complex numerical methods or approximations.

\begin{remark}[Prior information in Beta-Bernoulli model]
Comparing the forms of the prior and posterior distributions, we see that the hyper-parameter $a$ can be interpreted as representing the number of ``prior successes" (i.e., outcomes equal to 1), while $b$ corresponds to the number of ``prior failures" (i.e., outcomes equal to 0). The sum $a+b$ reflects the strength or confidence in the prior information---effectively acting like a prior sample size.
\end{remark}

\index{Bayesian estimator}
\index{Point estimator}
\index{Method of moments}
\index{Maximum likelihood estimation}
\index{Maximum likelihood estimator}
%\begin{remark}[Bayesian estimator]
In the Beta-Bernoulli example, similar to the maximum likelihood estimator or the method of moments (MoM)---which uses moment information to estimate model parameters---the Bayesian framework can also be used to obtain  estimates. However, instead of directly estimating a single value, the Bayesian approach provides a full posterior distribution over the parameter of interest: $p(\theta \mid \mathcalX)$.

When making predictions for new data observations, we do not use a fixed parameter value from the model $p(x_{n+1} \mid \theta)$ directly. Instead, we marginalize out the uncertainty in $\theta$ by integrating over the posterior distribution:
\begin{equation}
p(x_{n+1} \mid  \mathcalX) = \int p(x_{n+1} \mid  \theta) p(\theta \mid  \mathcalX) d\theta.\nonumber
\end{equation}
In other words, $x_{n+1}$ depends on $\mathcalX$. The observed data $\mathcalX$ provide information on $\theta$, which in turn provides information on $x_{n+1}$ (i.e., $\mathcalX \rightarrow \theta \rightarrow x_{n+1}$).
%\end{remark}

\begin{example}[Amount of data matters]\label{example:amountofdata}
Bayesian methods can be advantageous in cases of small sample sizes or sparse data, where traditional frequentist methods may encounter difficulties.
Suppose we have three observations for the success in  a Bernoulli experiment:
\begin{enumerate}[(1).]
\item  10 out of 10 trials are observed to be success (1's);
\item  48 out of 50 trials are observed to be success (1's);
\item  186 out of 200 trials  are observed to be success (1's).
\end{enumerate}

A common frequentist estimate of the success probability would be 100\%, 96\%, and 93\% for cases 1, 2, and 3, respectively. However, in case 1, an observation based on only 10 trials may be unreliable, as such a small sample size makes the estimate more sensitive to noise.

Now suppose we put a $\betadist(1,1)$ (a uniform distribution, see Figure~\ref{fig:dists_beta}) prior on the Bernoulli distribution parameter. 
Then the posterior probability of success for each case becomes $\frac{11}{12}=91.6\%$, $\frac{49}{52}=94.2\%$, and $\frac{187}{202}=92.6\%$, respectively. 
Interestingly, under this Bayesian approach, the estimated success probability for case 1 is actually lower than for case 2, despite both having perfect success rates in the observed data. This reflects the influence of the small sample size in case 1, which results in greater shrinkage toward the prior.

This Bayesian perspective naturally incorporates both the amount of data and the observed average into the final estimate. The specific form shown here is known as Laplace's rule of succession \citep{ollivier2015laplace}, which adjusts the observed frequency by adding one to both the count of successes and failures. This ``add-one" rule avoids assigning zero probability to unseen events and corresponds to using a uniform prior in a Bayesian framework.
\end{example}
\index{Laplace's add-one rule}

\index{Bayesian estimator}
% TODO: BOOK Fahrmeir: 
\begin{remark}[Why Bayes?]
The previous example illustrates that Bayesian models incorporate prior information about the parameters, making them particularly effective for regularizing regression problems when data are limited. This is one reason why the Bayesian approach has attracted widespread attention over the decades.

In the Bayesian framework, the prior distribution $p(\theta)$ and the likelihood function $p(x\mid \theta)$ together represent a rational individual's initial beliefs about the parameter $\theta$. Bayes' theorem then provides an optimal method for updating these beliefs in light of new data, resulting in the posterior distribution $p(\theta \mid x)$.

The prior $p(\theta)$  may not always accurately reflect true prior beliefs, and in such cases, it could be considered ``incorrect" or suboptimal. However, this does not necessarily render the resulting posterior $p(\theta \mid x)$ uninformative or useless. As famously stated: ``All models are wrong, but some are useful" \citep{box1987empirical}. If the prior $p(\theta)$  reasonably approximates our actual beliefs, then the resulting posterior $p(\theta \mid x)$ will also serve as a good approximation of the updated beliefs after observing the data.
\end{remark}

%\begin{figure}[h]
%	\centering  
%	\vspace{-0.35cm} 
%	\subfigtopskip=2pt 
%	\subfigbottomskip=6pt 
%	\subfigcapskip=-2pt 
%	%\includegraphics[width=0.35\textwidth]{imgs/bmf_bid_GBT.pdf}
%	\subfigure[GBT.]{\includegraphics[width=0.4\textwidth]{imgs/bmf_bid_GBT.pdf} \label{fig:bmf_bid_GBT}}
%	\subfigure[GBTN.]{\includegraphics[width=0.4\textwidth]{imgs/bmf_bid_GBTN.pdf} \label{fig:bmf_bid_GBTN}}
%	\caption{}
%	\label{fig:bmf_bids}
%\end{figure}

\begin{figure}[h!]
\centering
\includegraphics[width=0.35\textwidth]{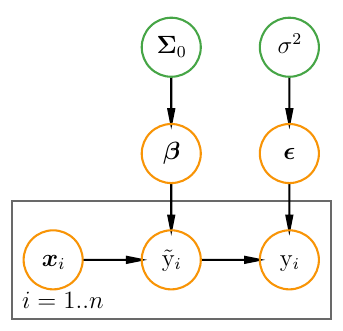}
\caption{Graphical representation of the Bayesian linear model with a zero-mean prior. Orange circles represent observed and latent variables, green circles denote prior variables, and plates represent repeated variables. 
In the graph, $\bbeta \sim \normal(\bzero, \bSigma_0)$, $\bepsilon \sim \normal(\bzero, \sigma^2 \bI)$, $\widetilde{\ry}_i = \bx_i^\top\bbeta$, and $\ry_i  = \widetilde{\ry}_i+\epsilon_i$.
}
\label{fig:blm_zeromean}
\end{figure}

\index{Zero-mean prior}
\section{Bayesian Linear Model: Zero-Mean Prior}\label{sec:bayesian-zero-mean}
We now introduce the application of Bayesian methods to linear regression models. 
Consider the standard linear model: 
$$
\rvy = \bX\bbeta + \bepsilon, 
$$ 
where $\bepsilon \sim \normal(\bzero, \sigma^2 \bI)$ and $\sigma^2$ is fixed. 
As discussed in Section~\ref{sec:lr-gaussian-noise}, the assumption of additive Gaussian noise leads naturally to a normal likelihood function. Let $\mathcalX=\mathcalX (\bx_{1:n})= \{\bx_1, \bx_2, \ldots, \bx_n\}$ be the observations of $n$ data points, and $\by=[y_1,y_2,\ldots, y_n]$ contains the corresponding responses.
Then, the joint likelihood is given by:
\begin{equation}
\mathrm{Likelihood} = \by \mid  \bX, \bbeta, \sigma^2 \sim \normal(\bX\bbeta, \sigma^2\bI). \nonumber
\end{equation}
Now suppose we place a multivariate Gaussian prior with zero mean on the weight vector:
\begin{equation}
\mathrm{Prior} = 	\bbeta \sim \normal(\bzero, \bSigma_0). \nonumber
\end{equation}
The graphical representation of this Bayesian linear model is shown in Figure~\ref{fig:blm_zeromean}. Applying Bayes' rule, which states that
 ``$\mathrm{Posterior} \propto \mathrm{Likelihood} \times \mathrm{Prior} $," we obtain the posterior distribution:
$$
\begin{aligned}
\mathrm{Posterior}
&= p(\bbeta\mid \by,\bX, \sigma^2)
\propto p(\by\mid \bX, \bbeta, \sigma^2) p(\bbeta \mid  \bSigma_0) \\
&= \frac{1}{(2\pi \sigma^2)^{n/2}} \exp\left\{-\frac{1}{2\sigma^2} (\by-\bX\bbeta)^\top(\by-\bX\bbeta)\right\}
\frac{1}{(2\pi)^{p/2}\abs{\bSigma_0}^{1/2}}\exp\left\{-\frac{1}{2} \bbeta^\top\bSigma_0^{-1}\bbeta\right\} \\
&\propto \exp\left\{-\frac{1}{2} (\bbeta - \bbeta_1)^\top \bSigma_1^{-1} (\bbeta - \bbeta_1)\right\}, 
\end{aligned}
$$
where $\bSigma_1 \triangleq \left(\frac{1}{\sigma^2} \bX^\top\bX + \bSigma_0^{-1}\right)^{-1}$ and $\bbeta_1 \triangleq  \left(\frac{1}{\sigma^2}\bX^\top\bX + \bSigma_0^{-1}\right)^{-1}(\frac{1}{\sigma^2}\bX^\top\by)$. 
Thus, the posterior distribution is also multivariate normal (i.e., it belongs to the same family as the prior), confirming that the Gaussian prior is conjugate to the Gaussian likelihood:
\begin{equation}
\mathrm{Posterior} = \bbeta\mid \by,\bX, \sigma^2 \sim \normal(\bbeta_1, \bSigma_1). \nonumber
\end{equation}

Note that we use $\{\bbeta_1,\bSigma_1\}$ to denote the posterior mean and posterior covariance in the zero-mean prior model. Similarly, the posterior mean and posterior covariance in semi-conjugate prior and full conjugate prior models will be denoted by $\{\bbeta_2,\bSigma_2\}$ and $\{\bbeta_3,\bSigma_3\}$, respectively (see later sections).

\subsection{Connection to Non-Bayesian Models}
We demonstrate that the Bayesian linear model with a zero-mean prior generalizes certain standard Gauss-Markov linear models.

\paragrapharrow{Connection to OLS.}
Importantly, in the Bayesian framework, there is no strict requirement for the design matrix $\bX$ to have full rank. This allows for more flexibility compared to classical linear regression.
However, if $\bX$ does have full column rank and we consider the limiting case where the prior becomes highly concentrated around zero, i.e., $\bSigma_0 \rightarrow \bzero$, then the posterior mean converges to the OLS estimate:
$$
\bbeta_1 \rightarrow \widehat{\bbeta} = (\bX^\top\bX)^{-1}\bX\by 
\qquad\text{as}\qquad 
\bSigma_0 \rightarrow \bzero.
$$
In this limit, the MAP estimate from the Bayesian model coincides with the OLS estimate.
Furthermore, the posterior distribution becomes:
$$
\bbeta\mid \by,\bX, \sigma^2 \sim \normal (\widehat{\bbeta}, \sigma^2(\bX^\top \bX)^{-1})
\qquad\text{as}\qquad 
\bSigma_0 \rightarrow \bzero,
$$
which shares the same form as the OLS estimator $\widehat{\bbeta} \sim \normal(\bbeta, \sigma^2(\bX^\top \bX)^{-1})$ under  Gaussian disturbances (see Theorem~\ref{theorem:samplding_dist_lse_gaussian}).

\index{Ridge regression}
\paragrapharrow{Connection to ridge regression.}
In the context of least squares approximation,   the utilization of $\bX\bbeta$ for approximating $\by$ introduces two potential issues: the risk of overfitting and the possibility of $\bX$ lacking full rank. 
Ridge regression addresses these concerns by regularizing large values of $\bbeta$, thereby favoring simpler models. 
Instead of minimizing the squared error along: $\normtwo{\by-\bX\bbeta}^2$, ridge regression minimizes the following objective function: $\normtwo{\by-\bX\bbeta}^2+\lambda\normtwo{\bbeta}^2$, where $\lambda$ is a hyper-parameter that controls the strength of the regularization and can be tuned as needed:
\begin{equation}
\mathop{\arg\min}_{\bbeta}{(\by-\bX\bbeta)^\top(\by-\bX\bbeta) + \lambda \bbeta^\top\bbeta}. \nonumber
\end{equation}
By differentiating this expression with respect to $\bbeta$ and setting the gradient to zero, we obtain the closed-form estimate:
\begin{equation}
\widehat{\bbeta}_{ridge} = (\bX^\top\bX + \lambda \bI)^{-1} \bX^\top\by, \nonumber
\end{equation}
in which case, $(\bX^\top\bX + \lambda \bI)$ is always invertible even when $\bX$ does not have full rank. 

Now consider the Bayesian linear model with a zero-mean prior. If we set $\bSigma_0 = \bI$, the posterior mean becomes
$\bbeta_1 = (\bX^\top\bX + \sigma^2 \bI)^{-1} \bX^\top\by$, and the posterior covariance is $\bSigma_1 = (\frac{1}{\sigma^2}\bX^\top\bX+ \bI)^{-1}$. Since $\mathrm{Posterior} = \bbeta\mid \by,\bX, \sigma^2 \sim \normal(\bbeta_1, \bSigma_1)$, the MAP estimate of  $\bbeta$ becomes $\bbeta = \bbeta_1 =  (\bX^\top\bX + \sigma^2 \bI)^{-1} \bX^\top\by$. 
This expression matches the ridge regression estimate when we identify $\sigma^2$ with $\lambda$. 
Therefore, ridge regression can be viewed as a special case of the Bayesian linear model with a zero-mean Gaussian prior.
The Bayesian framework provides a probabilistic interpretation of ridge regression: it corresponds to finding the most probable value of $\bbeta$ given the data and the prior assumptions---namely, that the coefficients are centered at zero with unit variance.

\index{Zeller's $g$-prior}
\subsection{Zeller's $g$-Prior and Variable Transformation}\label{section:gprior-zeromean}

As an illustrative example, consider modeling an individual's weight based on various human characteristics, where one variable in the input matrix $\bX$ represents the person's height in meters.
If this variable is instead expressed in centimeters, the underlying relationship remains unchanged. We can account for this change by simply dividing the corresponding coefficient in $\bbeta$ by 100, effectively converting centimeters back to meters.

More generally, suppose the input matrix $\bX$ undergoes a linear transformation such that $\widetilde{\bX} = \bX\bP$, where $\bP$ is a $p\times p$ invertible matrix; let the corresponding model parameter become $\widetilde{\bbeta}$. Then we have:
$$
\begin{aligned}
\by =  \bX \bbeta = \widetilde{\bX} \widetilde{\bbeta}= \bX\bP\widetilde{\bbeta}.
\end{aligned}
$$
According to the \textit{principle of invariance}, the posterior distributions of $\bbeta$ and $\bP \widetilde{\bbeta}$ should be equivalent---meaning our inference should not depend on how the input features are scaled or transformed.
From earlier results, the posterior distribution of $\bbeta$ given $\bX$ is:
\begin{equation}
\bbeta\mid \by,\bX, \sigma^2 \sim \normal\left(\left(\frac{1}{\sigma^2}\bX^\top\bX + \bSigma_0^{-1}\right)^{-1}\left(\frac{1}{\sigma^2}\bX^\top\by\right), \left(\frac{1}{\sigma^2} \bX^\top\bX + \bSigma_0^{-1}\right)^{-1}\right). \nonumber
\end{equation}
Similarly, for the transformed input matrix $\widetilde{\bX}$, the posterior distribution of $\bP \widetilde{\bbeta}$ becomes: 
\begin{equation}
\bP \widetilde{\bbeta}\mid \by,\widetilde{\bX}, \sigma^2 \sim \normal\left(\textcolor{mylightbluetext}{\bP}\left(\frac{1}{\sigma^2}\widetilde{\bX}^\top\widetilde{\bX} + \widetilde{\bSigma_0}^{-1}\right)^{-1}\left(\frac{1}{\sigma^2}\widetilde{\bX}^\top\by\right), \textcolor{mylightbluetext}{\bP}\left(\frac{1}{\sigma^2} \widetilde{\bX}^\top\widetilde{\bX} + \widetilde{\bSigma_0}^{-1}\right)^{-1}\textcolor{mylightbluetext}{\bP^\top}\right). \nonumber 
~\footnote{Affine transformation of multivariate normal distribution: if we assume that $\rvx\sim \normal(\bmu, \bSigma)$, then $\bA\rvx+\bb \sim \normal(\bA\bmu+\bb, \bA\bSigma\bA^\top)$ for deterministic matrix $\bA$ and vector $\bb$.}
\end{equation}

Following the principle of invariance, it can be shown that this condition holds when the prior covariance matrix takes the form: $\bSigma_0 = k(\bX^\top \bX)^{-1}$, where $k>0$ is a hyper-parameter. A popular specification of $k$ is to relate it to the noise variance $\sigma^2$ by $k=g\sigma^2$. This is called  Zeller's $g$-prior \citep{zellner1986assessing}.
Following the Bayesian linear model with a zero-mean prior, the posterior of $\bbeta$ is
\begin{equation}
\mathrm{Posterior} = \bbeta\mid \by,\bX, \sigma^2 \sim \normal(\bbeta_1, \bSigma_1). \nonumber
\end{equation}
where $\bSigma_1 = \frac{g\sigma^2}{g+1}(\bX^\top\bX )^{-1}$ and $\bbeta_1 =  \frac{g}{g+1}(\bX^\top\bX)^{-1}(\bX^\top\by)$. 

\index{Semi-conjugate prior}
\section{Bayesian Linear Model: Semi-Conjugate Prior Distribution}\label{sec:semiconjugate}

We will use the Gamma distribution  as the prior for the inverse variance (precision) parameter of a  Gaussian distribution. 
A formal definition of the Gamma distribution is provided in Definition~\ref{definition:gamma_distri}.
The choice of the Gamma distribution as the prior for precision is motivated by the rationale provided in \citet{kruschke2014doing}:
\begin{itemize}
\item ``Because of its role in conjugate priors for normal likelihood function, the Gamma distribution is routinely used as a prior for precision (i.e., inverse variance). But there is no logical necessity to do so, and modern Markov chain Monte Carlo (MCMC) methods permit more flexible specification of priors. Indeed, because precision is less intuitive than standard deviation, it can be more useful to give standard deviation a uniform prior that spans a wide range."
\end{itemize}

Building on the setup introduced in Section~\ref{sec:bayesian-zero-mean}, we now treat $\sigma^2$ as an unknown variable rather than a fixed constant. As before, the likelihood function is given by:
\begin{equation}
\mathrm{Likelihood} = \by \mid  \bX, \bbeta, \sigma^2 \sim \normal(\bX\bbeta, \sigma^2\bI). \nonumber
\end{equation}
We define a non-zero mean Gaussian prior on the weight vector $\bbeta$, along with a Gamma prior on the precision parameter $\gamma = 1/\sigma^2$:
\begin{equation}\label{equation:semiconju_blm}
\begin{aligned}
{\color{mylightbluetext}\mathrm{Prior:\,}} &\bbeta \sim \normal({\color{mylightbluetext}\bbeta_0}, \bSigma_0); \\
{\color{mylightbluetext}\mathrm{Hyperprior:\,}}&{\color{mylightbluetext}\gamma = 1/\sigma^2 \sim \gammadist(a_0, b_0)}, 
\end{aligned}
\end{equation}
where the modifications from the previous model are highlighted in blue.
The graphical representation of this Bayesian linear model is shown in Figure~\ref{fig:blm_semiconjugate}.

\begin{figure}[h!]
\centering
\includegraphics[width=0.458\textwidth]{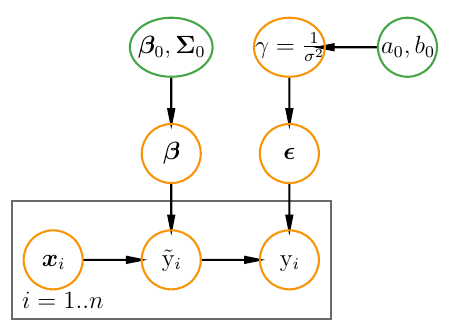}
\caption{Graphical representation of the Bayesian linear model with a semi-conjugate prior. Orange circles represent observed and latent variables, green circles denote prior variables, and plates represent repeated variables. 
The comma ``," in the variable represents ``and." In the graph, $\bbeta \sim \normal(\bbeta_0, \bSigma_0)$, $\gamma = 1/\sigma^2 \sim \gammadist(a_0, b_0)$, $\bepsilon \sim \normal(\bzero, \sigma^2 \bI)$, $\widetilde{\ry}_i = \bx_i^\top\bbeta$, and $\ry_i  = \widetilde{\ry}_i+\epsilon_i$.}
\label{fig:blm_semiconjugate}
\end{figure}

\paragrapharrow{Step 1, conditioning on $\sigma^2$.} Then, given $\sigma^2$, by  Bayes' theorem ``$\mathrm{Posterior} \propto \mathrm{Likelihood} \times \mathrm{Prior} $", we obtain the conditional posterior density of $\bbeta$:
\begin{equation}
\small
\begin{aligned}
&\mathrm{Posterior}
= p(\bbeta\mid \by,\bX, \sigma^2) 
\propto p(\by\mid \bX, \bbeta, \sigma^2) p(\bbeta \mid  \bbeta_0, \bSigma_0) \\
&=  \frac{1}{(2\pi \sigma^2)^{n/2}} \exp\left\{-\frac{1}{2\sigma^2} (\by-\bX\bbeta)^\top(\by-\bX\bbeta)\right\} 
 \frac{1}{(2\pi)^{p/2}\abs{\bSigma_0}^{1/2}}
\exp\left\{-\frac{1}{2} (\bbeta-\bbeta_0)^\top\bSigma_0^{-1}(\bbeta-\bbeta_0)\right\} \\
&\propto \exp\left\{-\frac{1}{2} (\bbeta - \bbeta_2)^\top \bSigma_2^{-1} (\bbeta - \bbeta_2)\right\}, \nonumber
\end{aligned}
\end{equation}
where the parameters are
$$
\begin{aligned}
\bSigma_2 &\triangleq \left(\frac{1}{\sigma^2} \bX^\top\bX + \bSigma_0^{-1}\right)^{-1};\\
\bbeta_2  &\triangleq \bSigma_2 (\bSigma_0^{-1}\bbeta_0+\frac{1}{\sigma^2}\bX^\top\by)= \left(\frac{1}{\sigma^2}\bX^\top\bX + \bSigma_0^{-1}\right)^{-1}\left(\textcolor{mylightbluetext}{\bSigma_0^{-1}\bbeta_0}+\frac{1}{\sigma^2}\bX^\top\by\right).
\end{aligned}
$$
Thus, the conditional posterior distribution is also Gaussian:
\begin{equation}\label{equation:posterior_beta2}
\mathrm{Posterior} = \bbeta\mid \by,\bX, \sigma^2 \sim \normal(\bbeta_2, \bSigma_2). \nonumber
\end{equation}

\paragrapharrow{Connection to the zero-mean prior model.}
We  highlight the relationship between the zero-mean prior model and the semi-conjugate prior model as follows:
\begin{enumerate}
\item  $\bSigma_0$ here is a fixed hyper-parameter.
\item We note that $\bbeta_1$ in Section~\ref{sec:bayesian-zero-mean}  is a special case of $\bbeta_2$ when $\bbeta_0=\bzero$. 
\item And if we assume further $\bX$ has full rank, when $\bSigma_0^{-1} \rightarrow \bzero$, $\bbeta_2$ approaches to $\bbeta_2 \rightarrow \widehat{\bbeta} = (\bX^\top\bX)^{-1}\bX\by$, which converges to the OLS estimate. 
\item  When $\sigma^2 \rightarrow \infty$, $\bbeta_2$ is approximately approaching to $\bbeta_0$, the prior expectation of parameter. However, in the zero-mean prior model, $\sigma^2 \rightarrow \infty$ will causes  $\bbeta_1$ to approach  $\bzero$.
\item \textit{Weighted average interpretation.} We can rewrite $\bbeta_2$ as:
\begin{equation}
\begin{aligned}
\bbeta_2 &=  \left(\frac{1}{\sigma^2}\bX^\top\bX + \bSigma_0^{-1}\right)^{-1}\left(\bSigma_0^{-1}\bbeta_0+\frac{1}{\sigma^2}\bX^\top\by\right) \\
&= \left(\frac{1}{\sigma^2}\bX^\top\bX + \bSigma_0^{-1}\right)^{-1} \bSigma_0^{-1}\bbeta_0 + \left(\frac{1}{\sigma^2}\bX^\top\bX + \bSigma_0^{-1}\right)^{-1} \frac{\bX^\top\bX}{\sigma^2} (\bX^\top\bX)^{-1}\bX^\top\by \\
&=(\bI-\bA)\bbeta_0 + \bA \widehat{\bbeta}, \nonumber
\end{aligned}
\end{equation}
where $\widehat{\bbeta}=(\bX^\top\bX)^{-1}\bX^\top\by$ is the OLS estimate of $\bbeta$, and $\bA\triangleq (\frac{1}{\sigma^2}\bX^\top\bX + \bSigma_0^{-1})^{-1} \frac{\bX^\top\bX}{\sigma^2}$. 
This shows that the posterior mean of  $\bbeta$ is a weighted average of the prior mean and the OLS estimate of $\bbeta$. 
Consequently, if we set the prior parameter $\bbeta_0 = \widehat{\bbeta}$, the posterior mean of $\bbeta$ becomes exactly $\widehat{\bbeta}$.
\end{enumerate}

\paragrapharrow{Step 2, conditioning  on $\bbeta$.} Given $\bbeta$, we again apply Bayes' theorem to obtain the posterior distribution of the precision parameter $\gamma$:
\begin{equation}
\begin{aligned}
\mathrm{Posterior}
&= p(\gamma=\frac{1}{\sigma^2}\mid \by,\bX, \bbeta) 
\propto p(\by\mid \bX, \bbeta, \gamma) p(\gamma \mid  a_0, b_0) \\
&=  \frac{\gamma^{n/2}}{(2\pi )^{n/2}} \exp\left\{-\frac{\gamma}{2} (\by-\bX\bbeta)^\top(\by-\bX\bbeta)\right\}  \times \frac{{b_0}^{a_0}}{\Gamma(a_0)} \gamma^{a_0-1} \exp(-b_0 \gamma) \\
&\propto \gamma(a_0+\frac{n}{2}-1) \exp\left\{-\gamma\left[b_0+\frac{1}{2}(\by-\bX\bbeta)^\top(\by-\bX\bbeta)\right]\right\}. \nonumber
\end{aligned}
\end{equation}
Therefore, the conditional posterior of $\gamma$ follows a Gamma distribution:
\begin{equation}\label{equation:posterior_gamma_sigma2}
\mathrm{posterior\,\, of\,\,} \gamma \mathrm{\,\,given\,\,} \bbeta  = \gamma\mid \by,\bX, \bbeta \sim \gammadist\left(a_0+\frac{n}{2}, \big[b_0+\frac{1}{2}(\by-\bX\bbeta)^\top(\by-\bX\bbeta)\big]\right). 
\end{equation}

\paragrapharrow{Prior information on the noise.}
We can interpret the prior on $\gamma$ intuitively as follows:
\begin{enumerate}
\item We notice that the prior mean and posterior mean of $\gamma$ are $\Exp[\gamma]=\frac{a_0}{b_0}$ and $\Exp[\gamma \mid \bbeta]=\frac{a_0 + \frac{n}{2}}{b_0 +\frac{1}{2}(\by-\bX\bbeta)^\top(\by-\bX\bbeta)}$, respectively. 
This suggests that the internal meaning of $2 a_0$ is the effective sample size of the prior information about the noise variance $\sigma^2 = \frac{1}{\gamma}$. 

\item As we assume $\rvy=\bX\bbeta +\bepsilon$, where $\bepsilon \sim \normal(\bzero, \sigma^2\bI)$, then $\frac{(\rvy-\bX\bbeta)^\top(\rvy-\bX\bbeta)}{\sigma^2} \sim \chi_{(n)}^2$ and $\Exp[\frac{1}{2}(\by-\bX\bbeta)^\top(\by-\bX\bbeta)] = \frac{n}{2}\sigma^2$. So the latent meaning of $\frac{b_0}{a_0}$ is the prior variance of the noise.

\item Some textbooks explicitly express $\gamma \sim  \gammadist(n_0/2, n_0\sigma_0^2/2)$  (in which case, $n_0$ is the prior sample size, and $\sigma_0^2$ is the prior variance). While this form may seem arbitrary at first glance, it provides an intuitive way to encode prior beliefs about the magnitude and uncertainty of the noise variance.
\end{enumerate}

\index{Gibbs sampling}
\index{Gibbs sampler}
\subsection{Gibbs Sampling with Two Variables}
Gibbs sampling was first introduced by Turchin \citep{turchin1971computation}, and later reintroduced by the Geman brothers in the context of image restoration \citep{geman1984stochastic}. 
The Geman brothers named the algorithm after the physicist J. W. Gibbs, some eight decades after his death, in reference to an analogy between the sampling algorithm and statistical physics.

Gibbs sampling is particularly useful when the joint distribution is not explicitly known or is difficult to sample from directly, but the conditional distributions of each variable are known and easy to sample from. A Gibbs sampler iteratively generates a sample for each parameter or variable, conditioned on the current values of all other parameters or variables. Therefore, it operates in a componentwise manner.

For example, given some data $\mathcalX$ and a probability distribution $p(\bbeta \mid  \mathcalX, \balpha)$, parameterized by $\bbeta = \{\beta_1, \beta_2, \ldots, \beta_p\}$. 
In this case, we can sequentially draw samples from the full conditional distributions:
\begin{equation}
\beta_i^{(t)} \sim p(\beta_i \mid  \bbeta_{-i}^{(t-1)}, \mathcalX, \balpha),
\end{equation}
where $\bbeta_{-i}^{(t-1)}$ is all current values of $\bbeta$ in $(t-1)$-th iteration except for $\beta_i$. 
If the sampling continues long enough, the resulting values of $\beta_i$ will approximate random samples from the target (posterior) distribution $p(\bbeta\mid \mathcalX, \balpha)$

In deriving a Gibbs sampler, it is often helpful to observe that
\begin{equation}
p(\beta_i \mid  \bbeta_{- i}, \mathcalX)
= \frac{
p(\beta_1, \beta_2, \ldots,\beta_p, \mathcalX)
}{
p(\bbeta_{- i}, \mathcalX)
} \propto p(\beta_1, \beta_2, \ldots,\beta_p, \mathcalX).
\end{equation}
That is, the conditional distribution is proportional to the joint distribution. 
This observation allows us to simplify computations by ignoring constant terms in the joint distribution that do not involve the parameter being sampled, e.g., we can discard the terms w.r.t. to $\bbeta_{- i}$ and only pay attention to $\beta_i$ when we want to sample from $p(\beta_i \mid  \bbeta_{- i}, \mathcalX)$.

%Suppose $p(X,Y)$ is a probability density function or probability mass function that is difficult to sample from. Suppose further that we can easily sample from the conditional distribution $p(Y\mid X)$ and $p(X\mid Y)$. The Gibbs sampler proceeds as following: set $X$ and $Y$ to some starting values, then sample $X\mid Y$, then sample $Y\mid X$, then sample $X\mid Y$, and so on. The process is defined as follows:
%0. Set ($x_0$, $y_0$) to some starting values;
%
%1. Sample $x_1$ from conditional distribution $p(x\mid y_0)$;
%
%$\,\,\,\,\,\,$Sample $y_1$ from conditional distribution $p(y\mid x_1)$;
%
%2. Sample $x_2$ from conditional distribution $p(x\mid y_1)$;
%
%$\,\,\,\,\,\,$Sample $y_2$ from conditional distribution $p(y\mid x_2)$;

To illustrate briefly, suppose we have a \textit{bivariate} joint distribution $p(\beta_1,\beta_2\mid \mathcalX)$. A Gibbs sampler would iteratively draw samples from the two full conditionals: first from $p(\beta_1\mid \beta_2,\mathcalX)$, then from $p(\beta_2\mid \beta_1,\mathcalX)$.
This iterative procedure generates a sequence of realizations for $\beta_1$ and $\beta_2$:
\begin{equation}
(\beta_1^0, \beta_2^0), \, (\beta_1^1, \beta_2^1), \, (\beta_1^2, \beta_2^2),\,  \ldots \nonumber
\end{equation}
which converges in distribution to the joint distribution $p(\beta_1, \beta_2 \mid \mathcalX)$. 
For further reading on Gibbs sampling, see \citet{turchin1971computation, geman1984stochastic, muller2004nonparametric, rencher2008linear, hoff2009first, gelman2013bayesian, kruschke2018bayesian}.

Using this Gibbs sampling technique, we can construct a Gibbs sampler for the Bayesian linear model with a semi-conjugate prior discussed in Section~\ref{sec:semiconjugate}. The steps are as follows:
\begin{enumerate}[(1).]
\item  Set initial values to $\bbeta$ and $\gamma = \frac{1}{\sigma^2}$.
\item Update $\bbeta$: $\mathrm{Posterior} = \bbeta\mid \by,\bX, \gamma \sim \normal(\bbeta_2, \bSigma_2)$.
\item Update $\gamma$: $\mathrm{Posterior}  = \gamma\mid \by,\bX, \bbeta \sim \gammadist\left(a_0+\frac{n}{2}, [b_0+\frac{1}{2}(\by-\bX\bbeta)^\top(\by-\bX\bbeta)]\right)$.
\end{enumerate}

\index{Principle of invariance}
\index{Zeller's $g$-prior}
\subsection{Zeller's $g$-Prior}
Similar to the variable transformation problem discussed in  Section~\ref{section:gprior-zeromean}, suppose the input matrix $\bX$ is transferred as $\widetilde{\bX} = \bX\bP$ given some $p\times p$ nonsingular matrix $\bP$, in which case, the model parameter is $\widetilde{\bbeta}$. Then, we have 
$$
\by =  \bX \bbeta = \widetilde{\bX} \widetilde{\bbeta}= \bX\bP\widetilde{\bbeta}.
$$
According to the principle of invariance, the posterior distributions of $\bbeta$ and $\bP \widetilde{\bbeta}$ should be identical.
Simple calculation can show that this condition is met if $\bbeta_0=\bzero, \bSigma_0 = g\sigma^2(\bX^\top \bX)^{-1}$.
Following the Bayesian linear model with a semi-conjugate prior, the posterior of $\bbeta$ is
\begin{equation}\label{equation:posterior_beta2_gprior}
\mathrm{Posterior} = \bbeta\mid \by,\bX, \sigma^2 \sim \normal(\bbeta_2, \bSigma_2). \nonumber
\end{equation}
where $\bSigma_2 = \frac{g\sigma^2}{g+1}(\bX^\top\bX )^{-1}$ and $\bbeta_2 =  \frac{g}{g+1}(\bX^\top\bX)^{-1}(\bX^\top\by)$.

\subsection*{\textbf{Derivation of $p(\by \mid   \bX,\sigma^2)$}}
Under the $g$-prior specified above, we now derive the conditional distribution $p(\by \mid   \bX,\sigma^2)$, which will be very useful for the Bayesian variable selection procedure.
Since $\rvy \mid   \bX,\bbeta, \sigma^2 \sim \normal(\bX\bbeta, \sigma^2\bI)$ and $\bbeta \mid  \bX,\sigma^2 \sim \normal(\bbeta_0, \bSigma_0)$, we have 
$$
\small
\begin{aligned}
&p(\by, \bbeta \mid   \bX,\sigma^2)
=p(\by \mid   \bX,\bbeta, \sigma^2) p(\bbeta \mid  \bX, \sigma^2)  \\
&=\frac{1}{(2\pi \sigma^2)^{\frac{n}{2}}} \exp\left\{-\frac{1}{2\sigma^2} (\by-\bX\bbeta)^\top(\by-\bX\bbeta)\right\}
%\qquad &(\text{$\rvy \mid   \bX,\bbeta, \sigma^2 \sim \normal(\bX\bbeta, \sigma^2\bI)$}) \\
 \times \frac{1}{(2\pi)^{\frac{p}{2}}\abs{\bSigma_0}^{1/2}}\exp\left\{-\frac{1}{2} \bbeta^\top\bSigma_0^{-1}\bbeta\right\}   \\
%\qquad &(\text{$\bbeta \mid  \bX,\sigma^2 \sim \normal(\bbeta_0, \bSigma_0)$}) \\
%&= \frac{1}{(2\pi \sigma^2)^{n/2}} \exp\left\{-\frac{1}{2\sigma^2}\by^\top\by\right\}
%\frac{\abs{\bSigma_2}^{1/2}}{\abs{\bSigma_0}^{1/2}}\exp\left\{\frac{1}{2}\bbeta_2^\top \bSigma_2^{-1}\bbeta_2 \right\}
%\times \frac{1}{(2\pi)^{p/2}\abs{\bSigma_2}^{1/2}}\exp\left\{-\frac{1}{2}(\bbeta-\bbeta_2)^\top\bSigma_2^{-1}(\bbeta-\bbeta_2)\right\},
&= \frac{1}{(2\pi \sigma^2)^{\frac{n}{2}}} \exp\left\{-\frac{\by^\top\by}{2\sigma^2} +\frac{\bbeta_2^\top \bSigma_2^{-1}\bbeta_2}{2}\right\}
\frac{\abs{\bSigma_2}^{1/2}}{\abs{\bSigma_0}^{1/2}}
 \frac{1}{(2\pi)^{\frac{p}{2}}\abs{\bSigma_2}^{1/2}}\exp\left\{-\frac{(\bbeta-\bbeta_2)^\top\bSigma_2^{-1}(\bbeta-\bbeta_2)}{2}\right\},
\end{aligned}
$$
where the parameter $\bbeta$ only appears in the third term, which corresponds to a multivariate normal distribution with mean $\bbeta_2$ and covariance $\bSigma_2$ (defined previously). Since this term integrates to 1 over all $\bbeta$, we can proceed to integrate it out. 
Since $\bSigma_2 = \frac{1}{g+1}\bSigma_0$ such that $\frac{\abs{\bSigma_2}^{1/2}}{\abs{\bSigma_0}^{1/2}} = \frac{1}{(g+1)^{p/2}}$. Therefore, we obtain 
\begin{equation}\label{equation:gprior-yxsigma}
\begin{aligned}
p(\by \mid   \bX,\sigma^2) 
&= \int p(\by, \bbeta \mid   \bX,\sigma^2) d\bbeta
=  \int p(\by \mid   \bX,\bbeta, \sigma^2) p(\bbeta \mid  \bX, \sigma^2) d\bbeta\\
&=\frac{1}{(2\pi \sigma^2)^{n/2}}\exp\left\{-\frac{1}{2\sigma^2}\by^\top\by\right\}\cdot
\frac{1}{(g+1)^{p/2}}\exp\left\{\frac{1}{2}\bbeta_2^\top \bSigma_2^{-1}\bbeta_2 \right\}\\
%&=\frac{1}{(2\pi \sigma^2)^{n/2}}\frac{1}{(g+1)^{p/2}}
%\exp\left(-\frac{1}{2\sigma^2}\by^\top\by+\frac{1}{2}\bbeta_2^\top \bSigma_2^{-1}\bbeta_2 \right)\\
&=\frac{1}{(2\pi \sigma^2)^{n/2}}\frac{1}{(g+1)^{p/2}}
\exp\left\{-\frac{r}{2\sigma^2} \right\},
\end{aligned}
\end{equation}
where $r \triangleq \by^\top\by -\by^\top (\frac{g}{g+1}\bX(\bX^\top\bX)^{-1}\bX^\top)\by$.

\index{Bayesian variable selection}
\index{Variable selection}
\subsection{Bayesian Variable Selection}
In Section~\ref{section:variable-selection}, we introduced variable selection using the $F$-test. 
An alternative approach can be achieved through the \textit{Bayesian variable selection} procedure.
\subsection*{\textbf{The Model}} \index{Mask vector}
Let $\bz=[z_1, z_2, \ldots, z_p] \in \real^p$ be a mask vector, where each component $z_j \in \{0,1\}$ for all $j \in \{1, 2, \ldots, p\}$. For each regression coefficient  $\beta_j$ in $\bbeta$, we set $\beta_j = z_j \times b_j$, where $b_j$ represents the original coefficient, and $\beta_j$ is the final coefficient that may be included or excluded from the model based on the value of $z_j$. 
In matrix form, this relationship can be written as:
\begin{equation}
\bbeta = \bz \odot\bb, 
\end{equation}
where $\odot$ denotes the Hadamard (element-wise) product. 
Then, the model with noise disturbance can be expressed as 
\begin{equation}
\rvy = \bX\bbeta +\bepsilon= \bX(\bz \odot\bb) +\bepsilon,
\end{equation}
where $\bepsilon \sim \normal(\bzero, \sigma^2\bI)$. 

In Bayesian variable selection, the goal is to estimate the posterior distribution of the mask vector $\bz$. 
That is, we aim to determine which variables are most likely to be included in the model given the observed data.
By Bayes' theorem, the posterior distribution of $\bz$ is proportional to the product of the prior and the likelihood:
\begin{equation}
p(\bz \mid  \by, \bX) \propto p(\bz) p(\by \mid  \bX, \bz).
\end{equation}
Alternatively, suppose we have two mask vectors $\bz_a$ and $\bz_b$. 
The ratio of their posterior probabilities is given by:
\begin{equation}
\begin{aligned}
\text{odds}(\bz_a, \bz_b \mid  \by, \bX) &= \frac{p(\bz_a \mid  \by, \bX)}{p(\bz_b \mid  \by, \bX)} &=&\,\,\,\, \frac{p(\bz_a)}{p(\bz_b)} &\times& \,\,\,\,\,\frac{p(\by \mid  \bX, \bz_a)}{p(\by \mid  \bX, \bz_b)} \\
&\text{Posterior odds} &=&\,\,\,\, \text{Prior odds} &\times&\,\,\,\,\, \text{Bayes factor}
\end{aligned}
\end{equation}
where the \textit{Bayes factor} 
quantifies how much the observed data favor the model associated with $\bz_a$ over the model associated with  $\bz_b$.

\subsection*{\textbf{Derivation of the Bayes Factor}}
We begin by writing out the marginal likelihood related to the Bayes factor:
\begin{equation}\label{equation:bae_der_fac}
\begin{aligned}
p(\by \mid  \bX, \bz) &= \int \int p(\by, \bbeta, \sigma^2 \mid  \bX, \bz) d\bbeta d\sigma^2 \\
&=\int \left(\int p(\by, \bbeta,  \mid  \bX, \bz, \sigma^2)d\bbeta\right)  p(\sigma^2)  d\sigma^2\\
&=\int p(\by \mid  \bX, \sigma^2, \bz)  p(\sigma^2)  d\sigma^2,
\end{aligned}
\end{equation}
where $p(\by \mid  \bX, \sigma^2, \bz)=\left(\int p(\by, \bbeta,  \mid  \bX, \bz, \sigma^2)d\bbeta\right)$ can be obtained from Equation~\eqref{equation:gprior-yxsigma} (under  Zeller's $g$-prior) by substituting $\bX$ by $\bX_z$, where we remove the variable $i$ if $z_i$=0.

We realize that $\gamma = \frac{1}{\sigma^2}$, and 
$$ p(\sigma^2) = p(\gamma \mid  a_0, b_0) =\frac{{b_0}^{a_0}}{\Gamma(a_0)} \gamma^{a_0-1} \exp(-b_0 \gamma).$$
Then, 
$$
\begin{aligned}
p(\by \mid  \bX, \sigma^2, \bz)  p(\sigma^2)&= \frac{1}{(2\pi \sigma^2)^{n/2}}\frac{1}{(g+1)^{\textcolor{mylightbluetext}{p_z}/2}}
\exp\left(-\frac{\textcolor{mylightbluetext}{r_z}}{2\sigma^2}  \right)\cdot \frac{{b_0}^{a_0}}{\Gamma(a_0)} \gamma^{a_0-1} \exp(-b_0 \gamma)\\
&= \frac{1}{(2\pi )^{n/2}}\frac{1}{(g+1)^{\textcolor{mylightbluetext}{p_z}/2}}
\frac{{b_0}^{a_0}}{\Gamma(a_0)} \gamma^{a_0+n/2-1} \exp\left(-(b_0+\frac{\textcolor{mylightbluetext}{r_z}}{2}) \gamma\right)\\
&= \frac{1}{(2\pi )^{n/2}}\frac{1}{(g+1)^{\textcolor{mylightbluetext}{p_z}/2}}
\frac{{b_0}^{a_0}}{\Gamma(a_0)}\frac{\Gamma(a_n)}{{b_n}^{a_n}} \cdot   \frac{{b_n}^{a_n}}{\Gamma(a_n)} \gamma^{a_n-1} \exp\left(-b_n \gamma\right)\\
&= \frac{1}{(2\pi )^{n/2}}\frac{1}{(g+1)^{\textcolor{mylightbluetext}{p_z}/2}}
\frac{{b_0}^{a_0}}{\Gamma(a_0)}\frac{\Gamma(a_n)}{{b_n}^{a_n}} \cdot   \gammadist(\gamma \mid  a_n,b_n),
\end{aligned}
$$
where $r_z=\by^\top\by -\by^\top (\frac{g}{g+1}\bX_z(\bX_z^\top\bX_z)^{-1}\bX_z^\top)\by$, $p_z$ is the number of 1's in $\bz$, $a_n = a_0+n/2$, $b_n= b_0 + \frac{r_z}{2}$, and $\gammadist(\gamma \mid  a_n,b_n)$ is the probability density function of the Gamma distribution with respect to $\gamma$ with parameters $a_n$ and $b_n$. Since $\gamma$ only appears in the last term $\gammadist(\gamma \mid  a_n,b_n)$, which integrates to 1; from \eqref{equation:bae_der_fac}, we have 
$$
p(\by \mid  \bX, \bz) = \frac{1}{(2\pi )^{n/2}}\frac{1}{(g+1)^{p_z/2}}
\frac{{b_0}^{a_0}}{\Gamma(a_0)}\frac{\Gamma(a_n)}{{b_n}^{a_n}}.
$$
\subsection*{\textbf{Same Prior Hyper-parameter}}
Similarly, under models $\bz_a$ and $\bz_b$, where we assume they share the same parameters of $a_0$ and $b_0$ in the two models, we have
$$
\frac{p(\by \mid  \bX, \bz_a)}{p(\by \mid  \bX, \bz_b)} = \frac{(g+1)^{\textcolor{mylightbluetext}{p_{z_b}}/2}}{(g+1)^{\textcolor{mylightbluetext}{p_{z_a}}/2}}
\left(\frac{2b_0+\textcolor{mylightbluetext}{r_{z_b}}}{2b_0+\textcolor{mylightbluetext}{r_{z_a}}}\right)^{a_0+\frac{n}{2}},
$$
where $p_{z_a}$ is the number of variables selected in  model $\bz_a$, and $p_{z_b}$ is the number of variables selected in  model $\bz_b$.

\subsection*{\textbf{Different Prior Hyper-parameter}}
We have previously mentioned that 
\begin{enumerate}[(1).]
\item  $2 a_0$ is the prior sample size for the noise $\sigma^2 = \frac{1}{\gamma}$. 
\item   $\frac{b_0}{a_0}$ is the prior variance of the noise.
\end{enumerate}
Suppose now, given the two models $\bz_a$ and $\bz_b$, we assume $2 a_0=1$ for both of the models (i.e., prior sample sizes for the noise are both 1), and set the $\frac{b_0}{a_0}$ to be the estimated residual variance under the least squares estimate for each model, say maximum likelihood estimators:
$$
\widehat{\sigma}^2_{z_a} = \frac{1}{n}\normtwo{\by-\bX_{z_a}\widehat{\bbeta}_{z_a}}^2
\qquad\text{and}\qquad 
\widehat{\sigma}^2_{z_b} = \frac{1}{n}\normtwo{\by-\bX_{z_b}\widehat{\bbeta}_{z_b}}^2,
$$
which are biased estimators for $\sigma^2$; see Section~\ref{section:mle-gaussian}. 
Alternatively, we could choose the unbiased estimators, which are divided by $n-p_{z_a}$ and $n-p_{z_b}$, respectively, rather than divided by $n$ (see Section~\ref{sec:dist_sse}). Then, we have 
\begin{equation}\label{equation:bayes-variable-sec-final}
\frac{p(\by \mid  \bX, \bz_a)}{p(\by \mid  \bX, \bz_b)} = \frac{(g+1)^{\textcolor{mylightbluetext}{p_{z_b}}/2}}{(g+1)^{\textcolor{mylightbluetext}{p_{z_a}}/2}}
\left(\frac{\textcolor{mylightbluetext}{\widehat{\sigma}^2_{z_a}}}{\textcolor{mylightbluetext}{\widehat{\sigma}^2_{z_b}}}\right)^{\frac{1}{2}}
\left(\frac{\textcolor{mylightbluetext}{\widehat{\sigma}^2_{z_b}}+\textcolor{mylightbluetext}{r_{z_b}}}{\textcolor{mylightbluetext}{\widehat{\sigma}^2_{z_a}}+\textcolor{mylightbluetext}{r_{z_a}}}\right)^{\frac{n+1}{2}}.
\end{equation}
Notice that the ratio of the marginal probabilities is essentially \textbf{a balance between the model complexity and goodness-of-fit}:
\begin{itemize}
\item A larger value of $p_{z_b}$ means the model $\bz_b$ has more selected variables (more complexity), which will make the ratio larger and penalize model $\bz_b$.

\item  However, a more complex model will make $r_{z_b}$ smaller, which in turn will make the ratio smaller and penalize model $\bz_a$.
\end{itemize}

\subsection*{\textbf{Gibbs Sampler}}
Given a current value $\bz=[z_1, z_2, \ldots, z_p]^\top$, a new value of the $j$-th variable $z_j$ is generated by sampling from $p(z_j \mid  \by, \bX, \bz_{-j})$, where $\bz_{-j}$ refers to the values of $\bz$ except the $j$-th element $z_j$. Specifically, we define the \textit{intermediate parameter}~\footnote{This intermediate parameter is quite useful in other contexts, e.g., Bayesian inference for interpolative decomposition \citep{lu2022bayesian, lu2022comparative}.}
$$
o_j =  \frac{p(z_j = 1 \mid  \by, \bX, \bz_{-j})}{p(z_j = 0 \mid  \by, \bX, \bz_{-j})} = \frac{p(z_j = 1)}{p(z_j = 0)} \times \frac{p(\by \mid  \bX, \bz_{-j}, z_j=1)}{p(\by \mid  \bX, \bz_{-j}, z_j=0)},
$$
where the last term can be obtained using Equation~\eqref{equation:bayes-variable-sec-final}. For simplicity, we can assume a uniform prior on $p(z_j)$: 
$$
p(z_j = 1)  =p(z_j = 0) = 0.5.
$$ 
Then, using this intermediate parameter, the full conditional probability of $z_j$ being equal to 1 can be obtained by 
\begin{equation}\label{equation:pzj1}
p(z_j = 1 \mid  \by, \bX, \bz_{-j})=\frac{o_j}{1+o_j}.
\end{equation}

Therefore, given the value of $\bz^{(k)}$ at the $k$-th step, we can generate the next set of values  $\{\bz^{(k+1)}, \gamma^{(k+1)}, \bbeta^{(k+1)}\}$ using the following steps:
\begin{enumerate}[(1).]
\item  Set initial values to $\bbeta$, $\gamma = \frac{1}{\sigma^2}$, and $\bz$ if $k$=1;

\item Update $\bz$: For $j\in \{1, 2, \ldots, p\}$ in random order, replace $z_j$ with a sample from $p(z_j = 1 \mid  \by, \bX, \bz_{-j})$ (Equation~\eqref{equation:pzj1});

\item Update $\bbeta$: $\bbeta\mid \by,\bX, \gamma, \bz \sim \normal(\bbeta_2, \bSigma_2)$, where $\bSigma_2 = \frac{g\sigma^2}{g+1}(\bX_z^\top\bX_z )^{-1}$ and $\bbeta_2 =  \frac{g}{g+1}(\bX_z^\top\bX_z)^{-1}(\bX_z^\top\by)$ (Equation~\eqref{equation:posterior_beta2_gprior});

\item Update $\gamma$: $\gamma\mid \by,\bX, \bbeta, \bz \sim \gammadist\left(a_0+\frac{n}{2}, [b_0+\frac{1}{2}\normtwo{\by-\bX_z\bbeta_z}^2]\right)$ (Equation~\eqref{equation:posterior_gamma_sigma2}).
\end{enumerate}

\begin{figure}[htp]
\centering  
\vspace{-0.55cm} 
\subfigtopskip=2pt 
\subfigbottomskip=2pt 
\subfigcapskip=-5pt 
\subfigure[Contour plot of normal-inverse-Gamma density by varying parameter $r$ (\textcolor{brightlavender}{purple}=low, \textcolor{mydarkyellow}{yellow}=high). ]{\label{fig:dists_normalinversegamma_varyingR}
\includegraphics[width=0.955\linewidth]{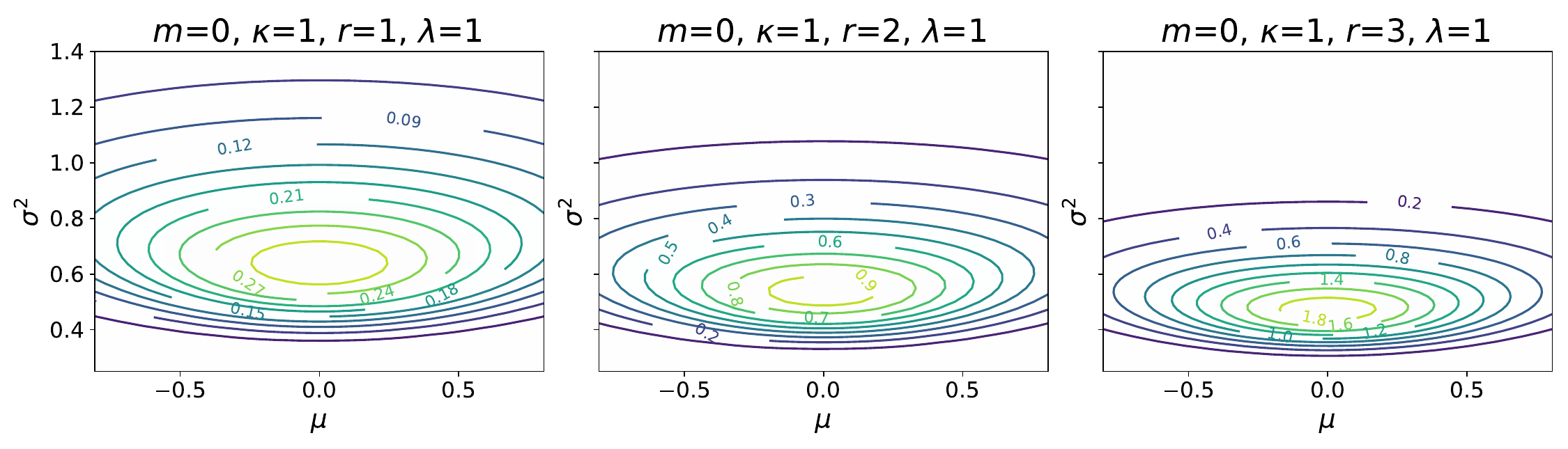}}
\subfigure[Contour plot of normal-inverse-Gamma density by varying parameter $\lambda$ (\textcolor{brightlavender}{purple}=low, \textcolor{mydarkyellow}{yellow}=high).]{\label{fig:dists_normalinversegamma_varyingLmabda}
\includegraphics[width=0.955\linewidth]{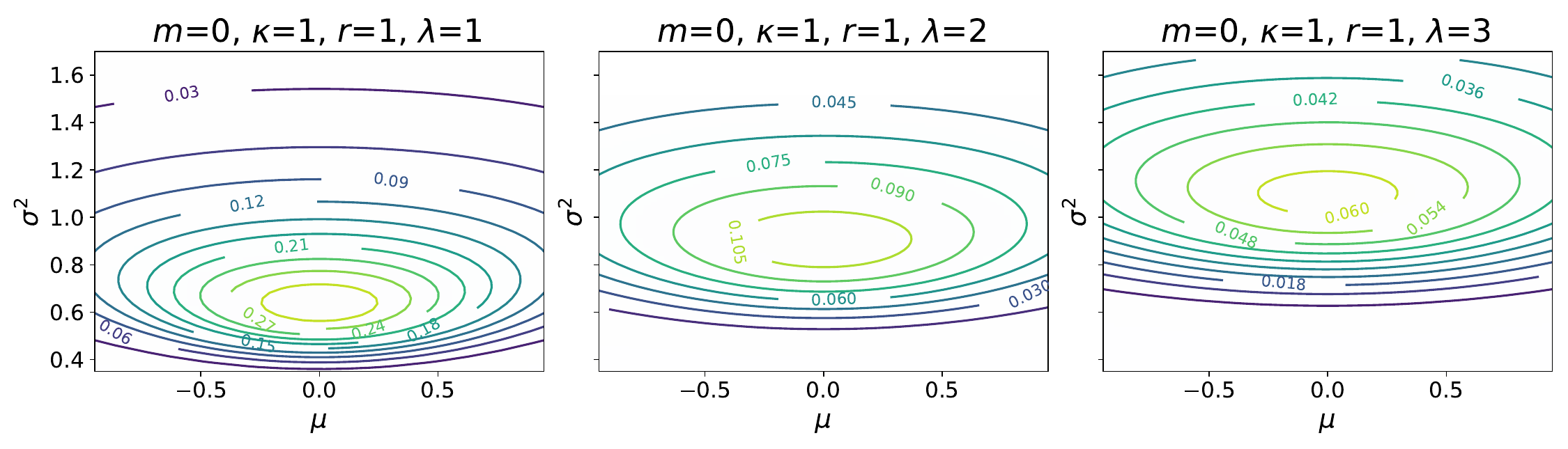}}
\subfigure[Contour plot of normal-inverse-Gamma density by varying parameter $\kappa$ (\textcolor{brightlavender}{purple}=low, \textcolor{mydarkyellow}{yellow}=high).]{\label{fig:dists_normalinversegamma_varyingKappa}
\includegraphics[width=0.955\linewidth]{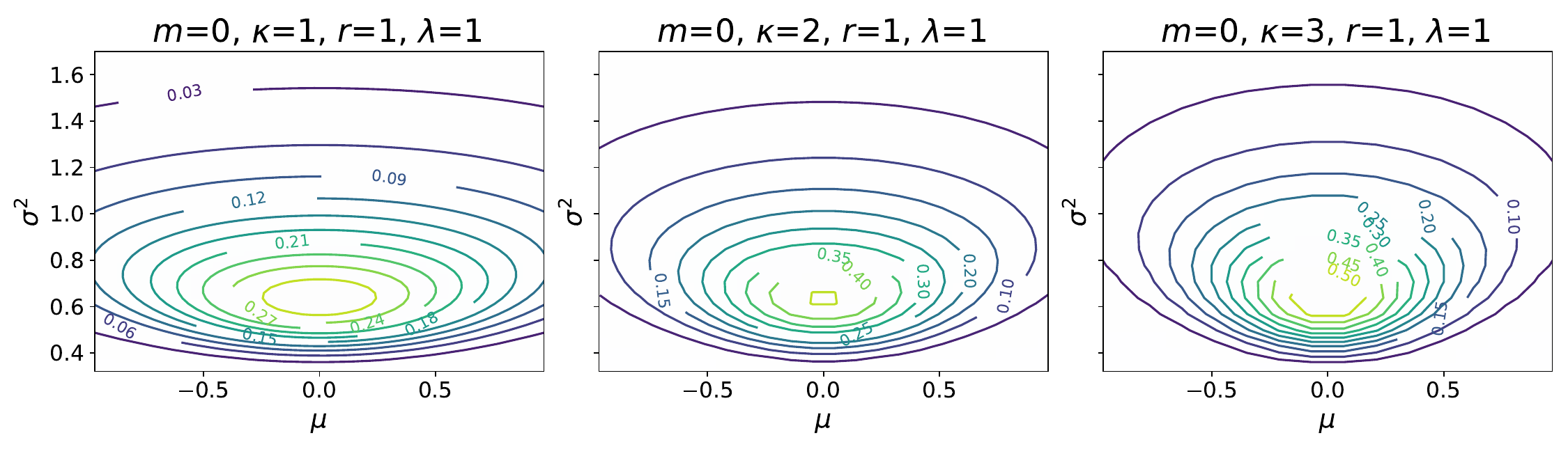}}
\subfigure[Contour plot of normal-inverse-Gamma density by varying parameter $m$ (\textcolor{brightlavender}{purple}=low, \textcolor{mydarkyellow}{yellow}=high).]{\label{fig:dists_normalinversegamma_varyingaM}
\includegraphics[width=0.955\linewidth]{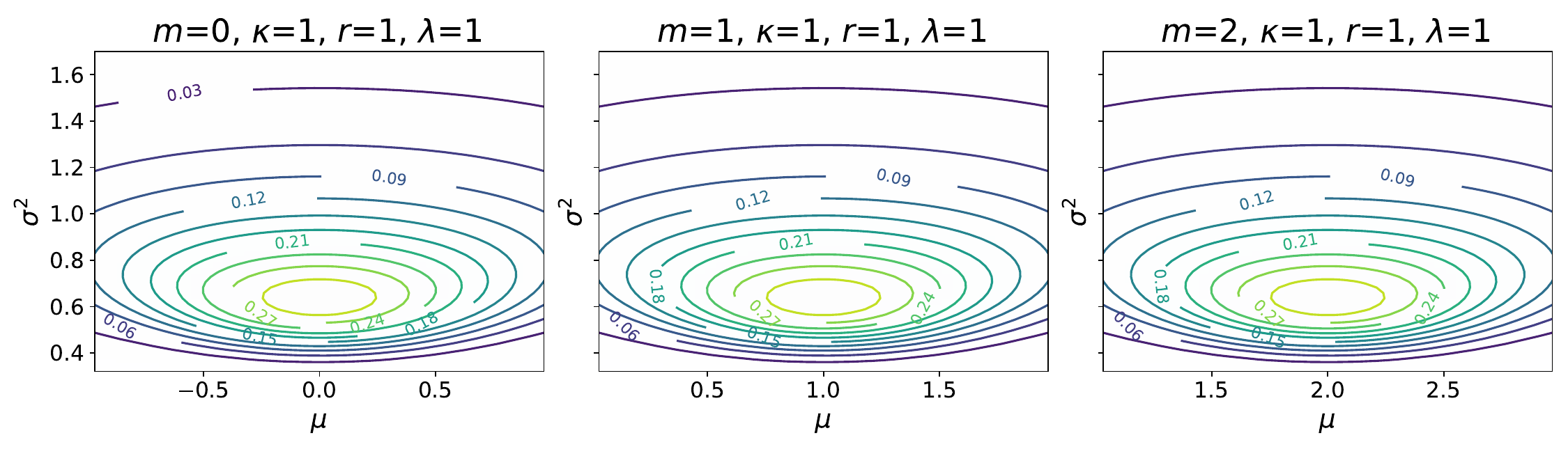}}
\caption{Normal-inverse-Gamma probability density functions by varying different parameters.}
\label{fig:dists_normalinversegamma_s}
\end{figure}

\index{Full conjugate prior}
\section{Bayesian Linear Model: Full Conjugate Prior}

\subsection{Normal-Inverse-Gamma Distribution}\label{section:niv_dist}

\index{Normal-inverse-Gamma distribution}
As we have seen that the Gamma density is a  conjugate prior for the precision parameter of a Gaussian distribution. The \textit{normal-inverse-Gamma (NIG)} distribution defined as follows is a joint conjugate prior for the mean and variance parameters of a Gaussian distribution.
\begin{definition}[Normal-Inverse-Gamma (NIG) Distribution]\label{definition:normal_inverse_gamma}
The joint density of  normal-inverse-Gamma distribution is a  density defined as
\begin{equation}
\begin{aligned}
&\gap \nig (\smu, \sigma^2 \mid m, \kappa, r, \lambda) 
= \mathcal{N} (\mu\mid m, \frac{\ssigma^2}{\kappa})  \cdot \inversegammadist (\ssigma^2 \mid r,  \lambda) \\
&=\frac{1}{Z_{\nig}(\kappa, r, \lambda)}  (\ssigma^2)^{-\frac{2r +3}{2}} \exp\left\{-\frac{1}{2 \ssigma^2}\left[\kappa(m-\mu)^2 + 2\lambda \right] \right\},  \\
\end{aligned}
\label{equation:uni_gaussian_prior-nig}
\end{equation}
where $\sigma^2, r, \lambda>0$, and $Z_{\nig}(\kappa, r, \lambda)$ is a normalizing constant:
\begin{equation}\label{equation:uni_gaussian_giw_constant-nig}
Z_{\nig}(\kappa, r, \lambda) = \frac{\Gamma(r)}{\lambda^{r}} \sqrt{\frac{2\pi}{\kappa}}.
\end{equation}
Figure~\ref{fig:dists_normalinversegamma_s} illustrates several probability density functions of the normal-inverse-Gamma  distribution, obtained by varying different parameter values.
\end{definition}

The normal-inverse-Gamma defines a conjugate  prior over the mean and variance parameters of a Gaussian distribution.
When the variance and mean parameters of the Gaussian distribution are not fixed with $n$ data points $\mathcalX=\{x_1, x_2, \ldots, x_n\}$ drawn i.i.d. from a normal distribution with mean $\mu$ and variance $\sigma^2$.
The normal-inverse-Gamma $\normalinversegamma(m_0, \kappa_0, r_0, \lambda_0)$ with $m_0\in\real$ and $r_0, \lambda_0, \kappa_0\in\real_+$ is a joint distribution on $\mu, \sigma^2$ by letting 
$$
\begin{aligned}
\mu \mid \sigma^2 &\sim \normal(m_0, \frac{\sigma^2}{\kappa_0})
\qquad \text{and}\qquad 
\sigma^2 \sim \inversegammadist(r_0, \lambda_0)
\end{aligned}
$$
With this prior, $\mu$ and $\sigma^2$ decouple, and the posterior conditional densities of $\mu$ and $\sigma^2$ are Gaussian and inverse-Gamma, respectively.
The joint p.d.f of NIG prior can be expressed as
$$
p(\mu, \sigma^2) = \normal(m_0, \frac{\sigma^2}{\kappa_0}) \cdot \inversegammadist(r_0, \lambda_0)
= \normalinversegamma(\mu, \sigma^2 \mid m_0, \kappa_0, r_0, \lambda_0).
$$

Again, by  Bayes' theorem ``$\mathrm{Posterior} \propto \mathrm{Likelihood} \times \mathrm{Prior} $," the  posterior of the $\mu$ and $\sigma^2$ parameters of a Gaussian distribution under the NIG prior is
\begin{equation}\label{equation:conjugate_nigamma_general}
\begin{aligned}
&\gap p(\mu, \sigma^2\mid \mathcalX, \bbeta ) \\
&\propto \normal(\mathcalX \mid \mu, \sigma^2)
\cdot  \normalinversegamma(\mu, \sigma^2 \mid \bbeta)
\propto \prod_{i=1}^{n}\normal(x_i\mid \mu, \sigma^2)
\cdot \normalinversegamma(\mu, \sigma^2 \mid  m_0, \kappa_0, r_0, \lambda_0)\\
%  = p(\mathcalX, \mu, \sigma^2 | m_0, \kappa_0, \alpha_0, \beta_0)\\
%&= (2\pi)^{-n/2}  (\ssigma^2)^{-n/2} \exp\left(-\frac{1}{2 \ssigma^2}  \left[  n \sum_{n=1}^n(x_n - \overline{x})^2 + n(\overline{x} - \smu)^2  \right] \right)\\
%&\gap \times \frac{1}{Z_{\nix}(\kappa_0, \nu_0, S_0)} (\ssigma^2)^{-(\nu_0/2 + 3/2)} \exp\left( -\frac{1}{2 \ssigma^2} \left[\nu_0 S_0 + \kappa_0(m_0-\mu)^2\right] \right)\\
&\stackrel{\dag}{=} \frac{C}{(\sigma^2)^{\frac{2r_0 + 3+n}{2}}} \exp\left\{-\frac{\left[  n(\overline{x} - \smu)^2 +  n S_{\overline{x}} \right]}{2 \ssigma^2}   \right\} 
\exp\left\{ -\frac{\left[2\lambda_0 + \kappa_0(m_0-\mu)^2\right]}{2 \ssigma^2}  \right\}\\
&\propto (\sigma^2)^{-\frac{2r_n + 3}{2}}\exp\left\{ -\frac{1}{2 \ssigma^2} \left[ \lambda_n + \kappa_n(m_n-\mu)^2\right] \right\}
\propto\nig(\mu, \sigma^2 \mid m_n, \kappa_n, r_n, \lambda_n),
\end{aligned}
\end{equation}
where $\bbeta\triangleq\{m_0, \kappa_0, r_0, \lambda_0\}$, $C\triangleq\frac{(2\pi)^{-n/2}}{Z_{\nig}(\kappa_0, r_0, \lambda_0)}$, the equality $(\dag)$ follows from \eqref{equation:uni_gaussian_likelihood}, and 
$$
\begin{aligned}
m_n &\triangleq\frac{\kappa_0 m_0 + n\overline{x}}{\kappa_n} = \frac{\kappa_0 }{\kappa_n}m_0 + \frac{n}{\kappa_n}\overline{x};
\qquad &\kappa_n&\triangleq \kappa_{0} +n;\\
r_n &\triangleq r_0 +\frac{n}{2};
\qquad &\lambda_n &\triangleq\lambda_0 +\frac{1}{2}(n S_{\overline{x}} + n\overline{x}^2 + \kappa_{0} m_0^2 -\kappa_nm_n^2)\\
&&&= \lambda_0+\frac{1}{2}\left(n S_{\overline{x}} + \frac{\kappa_0 n }{\kappa_{0}+n} (\overline{x} - m_0)^2\right),
\end{aligned}
$$
with $S_{\overline{x}}\triangleq\sum_{n=1}^n(x_n - \overline{x})^2$ and $\overline{x} \triangleq \frac{1}{n} \sum_{i=1}^{n}x_i$.
Note in the above derivation, we used the expression for the likelihood of a Gaussian distribution in \eqref{equation:uni_gaussian_likelihood}.
Furthers discussion on the posterior marginal likelihood for the NIG or NIX (normal-inverse-Chi-squared) priors can be found, for example, in \citet{lu2023bayesian}.

\subsection{Full Conjugate Prior Model}

%Note that the inverse-Gamma density is not simply the Gamma density with
%$x$ replaced by $\frac{1}{y}$. There is an additional factor of $y^{-2}$ from the Jacobian in the change-of-variables formula. \footnote{Which follows from the Jacobian in the change-of-variables formula. A short proof is provided here. Let $y=\frac{1}{x}$ where $y\sim \inversegammadist(r, \lambda)$ and $x\sim \gammadist(r, \lambda)$. Then, $f(y) |dy| = f(x) |dx|$ which results in $f(y) = f(x) |\frac{dx}{dy}| = f(x)x^2 \xlongequal{ \mathrm{y}=\frac{1}{x}} \frac{\lambda^r}{\Gamma(r)} y^{-r-1} exp(- \frac{\lambda}{y})$ for $y>0$. }

The joint conjugate prior on the mean and variance parameters of a Gaussian distribution provides insight into constructing a full conjugate prior for the Bayesian linear model.

To develop the full conjugate Bayesian model for linear regression, we place an inverse-Gamma prior (Definition~\ref{definition:inverse_gamma_distribution}) on the variance parameter.
Alternatively, placing a Gamma prior on the inverse variance (also known as precision), $\gamma=1/\sigma^2$, is mathematically equivalent to using an inverse-Gamma prior on $\sigma^2$ itself.

Consider the same setting as the semi-conjugate prior distribution discussed in Section~\ref{sec:semiconjugate}, the likelihood function---given the regression coefficient $\bbeta$, the design matrix $\bX$, the variance scale $\sigma^2$---is defined as follows:
$$
\begin{aligned}
\mathrm{Likelihood} 
&= \by \mid  \bX, \bbeta, \sigma^2 
\sim \normal(\bX\bbeta, \sigma^2\bI),
\end{aligned}
$$
which is identical to the likelihood density used in the zero-mean prior model (Section~\ref{sec:bayesian-zero-mean}) and the semi-conjugate model (Section~\ref{sec:semiconjugate}).
However, in this case, we specify:
\begin{itemize}
\item An Gaussian prior (the covariance matrix of which is not fixed this time) over the regression coefficient $\bbeta$, where the covariance matrix depends on the variance parameter $\sigma^2$.
\item An inverse-Gamma prior over the variance parameter $\sigma^2$, which is known as a \textit{hyperprior} since the variance parameter $\sigma^2$ itself is used as a parameter in the Gaussian prior; see Figure~\ref{fig:blm_fullconjugate} for an illustration.
\end{itemize}
Specifically, the prior distributions are defined as:
\begin{equation}
\begin{aligned}
{\color{mylightbluetext}\mathrm{Prior:\,}} &\bbeta\mid \sigma^2 \sim \normal(\bbeta_0, {\color{mylightbluetext}\sigma^2} \bSigma_0); \\
{\color{mylightbluetext}\mathrm{Hyperprior:\,}} &{\color{mylightbluetext}\sigma^2 \sim \inversegammadist(a_0, b_0)}, \nonumber
\end{aligned}
\end{equation}
where the blue-colored terms highlight the differences from earlier models.
A graphical representation of this Bayesian linear model is provided in Figure~\ref{fig:blm_fullconjugate}.

Note that, unlike in the semi-conjugate case---where we placed a Gamma prior on the precision $\gamma=1/\sigma^2$---here we directly use an inverse-Gamma prior on the variance $\sigma^2$. However, these two approaches are mathematically equivalent. This equivalence can be shown using the change-of-variables formula and computing the corresponding Jacobian ``determinant"; see Problem~\ref{prob:equiv_invg_ga_var}.

We can equivalently express the prior as what is also known as the \textit{normal-inverse-Gamma (NIG) distribution} (compare to Definition~\ref{definition:normal_inverse_gamma}):
$$
\begin{aligned}
\mathrm{Prior:} \,\bbeta,\sigma^2 
&\sim \nig(\bbeta_0, \bSigma_0, a_0, b_0) \\
&=  \normal(\bbeta_0, \sigma^2 \bSigma_0)\cdot \inversegammadist(a_0, b_0) ,
\end{aligned}
$$
where the normal distribution is multivariate, in contrast to the univariate version used in Definition~\ref{definition:normal_inverse_gamma}.

Once again, by applying  Bayes' theorem ``$\mathrm{Posterior} \propto \mathrm{Likelihood} \times \mathrm{Prior} $", we obtain the posterior distribution:
\begin{equation}
\begin{aligned}
\mathrm{Posterior}&= p(\bbeta,\sigma^2\mid \by,\bX) 
\propto p(\by\mid \bX, \bbeta, \sigma^2) p(\bbeta, \sigma^2 \mid  \bbeta_0, \bSigma_0, a_0, b_0) \\
&=  \frac{1}{(2\pi \sigma^2)^{n/2}} \exp\left\{-\frac{1}{2\sigma^2} (\by-\bX\bbeta)^\top(\by-\bX\bbeta)\right\} \\
&\,\,\,\,\,\, \times \frac{1}{(2\pi \sigma^2)^{p/2} \abs{\bSigma_0}^{1/2}} \exp\left\{-\frac{1}{2\sigma^2} (\bbeta - \bbeta_0)^\top\bSigma_0^{-1} (\bbeta - \bbeta_0)\right\} \\
&\,\,\,\,\,\, \times  \frac{{b_0}^{a_0}}{\Gamma(a_0)} \frac{1}{(\sigma^2)^{a_0+1}} \exp\{-\frac{b_0}{\sigma^2} \} \\
&\propto \frac{1}{(2\pi \sigma^2)^{p/2} }  \exp\left\{  \frac{1}{2\sigma^2} (\bbeta -\bbeta_3)^\top\bSigma_3^{-1}(\bbeta -\bbeta_3) \right\} \\
&\,\,\,\,\,\, \times \frac{1}{(\sigma^2)^{a_0 +\frac{n}{2}+1}} \exp\left\{-\frac{1}{\sigma^2} \big[b_0+\frac{1}{2} (\by^\top\by +\bbeta_0^\top\bSigma_0^{-1}\bbeta_0 -\bbeta_3^\top\bSigma_3^{-1}\bbeta_3) \big]\right\}, \nonumber
\end{aligned}
\end{equation}
where  
$$
\begin{aligned}
\bSigma_3 &\triangleq ( \bX^\top\bX + \bSigma_0^{-1})^{-1};\\
\bbeta_3  &\triangleq \bSigma_3(\bX^\top\by + \bSigma_0^{-1}\bbeta_0) = \left( \bX^\top\bX + \bSigma_0^{-1}\right)^{-1}(\bSigma_0^{-1}\bbeta_0 + \bX^\top\by).
\end{aligned}
$$ 
Let $a_n \triangleq a_0 +\frac{n}{2}+1$ and $b_n\triangleq b_0+\frac{1}{2} (\by^\top\by +\bbeta_0^\top\bSigma_0^{-1}\bbeta_0 -\bbeta_3^\top\bSigma_3^{-1}\bbeta_3) $. The posterior admits conjugacy and follows a NIG distribution:
\begin{equation}
\begin{aligned}
\mathrm{Posterior}&=  \bbeta, \sigma^2 \mid  \by, \bX \sim \nig(\bbeta_3, \bSigma_3, a_n, b_n). \nonumber
\end{aligned}
\end{equation}
\begin{figure}[h!]
\centering
\includegraphics[width=0.45\textwidth]{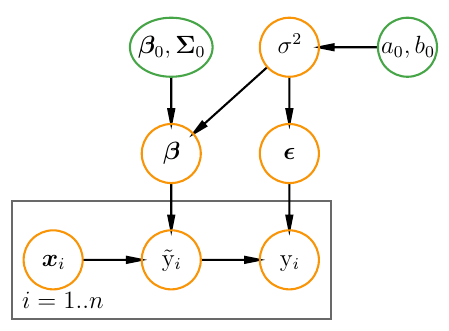}
\caption{Graphical representation of the Bayesian linear model with a full conjugate prior. Orange circles represent observed and latent variables, green circles denote prior variables, and plates represent repeated variables. 
The  comma ``," in the variable represents ``and." In the graph, $\bbeta \sim \normal(\bbeta_0, \sigma^2\bSigma_0)$, $\sigma^2 \sim \inversegammadist(a_0, b_0)$, $\bepsilon \sim \normal(\bzero, \sigma^2 \bI)$, $\widetilde{\ry}_i = \bx_i^\top\bbeta$, and $\ry_i  = \widetilde{\ry}_i+\epsilon_i$.}
\label{fig:blm_fullconjugate}
\end{figure}

\index{Semi-conjugate prior}
\paragrapharrow{Connection to  zero-mean prior and semi-conjugate prior models.}
We highlight the connection of the full conjugate model to the zero-mean prior and semi-conjugate prior models as follows:
\begin{enumerate}
\item In this model, $\bSigma_0$ is a fixed hyper-parameter that controls the strength and direction of the prior belief about the regression coefficients $\bbeta$

%TODO: We note that $\bbeta_1$ in Section~\ref{sec:bayesian-zero-mean} is a special case of $\bbeta_2$ when $\bbeta_0=\bzero$. 

\item If we assume further that the design matrix $\bX$ has full rank, then as $\bSigma_0^{-1} \rightarrow \bzero$, the posterior mean $\bbeta_3 \rightarrow \widehat{\bbeta} = (\bX^\top\bX)^{-1}\bX\by$, which converges to the OLS estimate. 

\item  When $b_0 \rightarrow \infty$, then $\sigma^2 \rightarrow \infty$ and $\bbeta_3$ is approximately approaching $\bbeta_0$, the prior expectation of parameter. Compared to $\bbeta_2$ in Section~\ref{sec:semiconjugate}, $\sigma^2 \rightarrow \infty$ will make $\bbeta_2$ approach  $\bbeta_0$, where $\sigma^2$ is a fixed hyper-parameter.

\item \textit{Weighted average interpretation.} We can rewrite $\bbeta_3$ as follows: 
\begin{equation}
\begin{aligned}
\bbeta_3 &=  (\bX^\top\bX + \bSigma_0^{-1})^{-1}(\bSigma_0^{-1}\bbeta_0+\bX^\top\by) \\
&= (\bX^\top\bX + \bSigma_0^{-1})^{-1} \bSigma_0^{-1}\bbeta_0 + (\bX^\top\bX + \bSigma_0^{-1})^{-1} (\bX^\top\bX) (\bX^\top\bX)^{-1}\bX^\top\by \\
&=(\bI-\bC)\bbeta_0 + \bC \widehat{\bbeta}, \nonumber
\end{aligned}
\end{equation}
where $\widehat{\bbeta}=(\bX^\top\bX)^{-1}\bX^\top\by$ is the OLS estimate of $\bbeta$, and the matrix $\bC$ is defined as $\bC\triangleq \left(\bX^\top\bX + \bSigma_0^{-1}\right)^{-1} (\bX^\top\bX)$. Therefore, the posterior mean of $\bbeta$ is a weighted average of the prior mean and the OLS estimate of $\bbeta$. Thus, if we set $\bbeta_0 = \widehat{\bbeta}$, the posterior mean of $\bbeta$ becomes exactly  $\widehat{\bbeta}$.

\item  From $a_n = a_0 +\frac{n}{2}+1$, we see that $2a_0$ can be interpreted as the effective sample size contributed by the prior on $\sigma^2$. This helps quantify the influence of the prior relative to the observed data.

% TODO: 
\item  $\bSigma_3^{-1} = \bX^\top\bX + \bSigma_0^{-1}$: The posterior inverse covariance matrix (i.e., the precision matrix) is equal to data inverse covariance $\bX^\top\bX$ + prior precision matrix. Hence, the posterior uncertainty reflects a balance between what the data tells us and what we believed a priori.
\end{enumerate}

%\newpage
%\chapter{Beyond Bayesian Approach: Gaussian Process Regression}
%\begingroup
%\hypersetup{
%linkcolor=structurecolor,
%linktoc=page,  % page: only the page will be colored; section, all, none etc
%}
%\minitoc \newpage
%\endgroup
%\index{Gaussian process}
\section{Beyond Bayesian Approach: Gaussian Process Regression}
%\lettrine{\color{caligraphcolor}I}
%In certain textbooks or college courses, discussions on the Gaussian process often occur without introducing the Bayesian linear model, leading to potential confusion. 
%It will become evident that the Gaussian process regression truly ``arises" from the Bayesian linear model with a zero-mean prior.

A \textit{Gaussian process (GP)} is a powerful nonparametric Bayesian tool used primarily for regression and classification tasks. At its core, a GP defines a distribution over functions, allowing us to reason about uncertainty in function estimation. Rather than specifying a fixed functional form as in parametric models, GPs generalize this idea by assuming that any finite set of function values follows a multivariate Gaussian distribution.

Importantly, Gaussian process regression can be derived as an extension of the Bayesian linear model with a zero-mean Gaussian prior on the coefficient parameter. When the feature space is implicitly defined via a kernel function---corresponding to the covariance function of the GP---the posterior distribution over functions naturally emerges without explicitly modeling parameters. This connection clarifies that Gaussian processes are not arbitrary constructs but have a firm grounding in Bayesian inference, often obscured when GPs are introduced independently of their parametric origins.

\index{Basis function}
\subsection*{Predictive Distribution of Bayesian Linear Model with Zero-Mean Prior}
Building on the Bayesian linear model with a zero-mean prior introduced in Section~\ref{sec:bayesian-zero-mean}, the predictive distribution $g_\ast = g_\ast(\bx_\ast)$ for a new data observation $\bx_\ast$ remains a Gaussian distribution:
\begin{equation}\label{equation:zeromean-predictive}
\begin{aligned}
p(g_\ast \mid  \bx_\ast \bX, \by, \sigma^2) &=\int p(g_\ast \mid  \bx_\ast , \bbeta)p(\bbeta\mid  \bX, \by, \sigma^2) d\bbeta \\ 
&=\normal\left(\frac{1}{\sigma^2}\bx_\ast^\top \bSigma_1\bX^\top\by, \bx_\ast^\top \bSigma_1\bx_\ast\right) \\
&=\normal(\bx_\ast^\top \bbeta_1, \bx_\ast^\top \bSigma_1\bx_\ast), 
\end{aligned}
\end{equation}
where the mean of the predictive distribution corresponds to the posterior mean of the weight vector  (i.e., $\bbeta_1$) multiplied by the new input $\bx_\ast^\top$, and the predictive variance is a quadratic form involving $\bx_\ast^\top $, indicating that the \textbf{predictive uncertainty} increases as the magnitude of the input grows.

However, this Bayesian linear model has limited expressive power and may struggle to capture complex patterns in the data. One way to address this limitation is by introducing a set of basis functions that map the original $p$-dimensional input space into a higher-dimensional space, say $q$-dimensional:
$$
\bx\in \real^{p} \rightarrow \bphi(\bx) \in \real^{q}
\qquad\text{and}\qquad 
\bX\in\real^{n\times p} \rightarrow \bPhi(\bX) \in \real^{n\times q}.
$$

Nonetheless, increasing the dimensionality introduces computational challenges. Specifically, computational complexity grows quadratically with the number of dimensions. For example, computing the predictive variance increases from  $\mathcalO(p^2+p)$ to $\mathcalO(q^2+q)$. Such quadratic growth becomes a significant issue when scalability and efficiency are crucial.

%However, this Bayesian linear model suffers from limited expressiveness, having problems in its ability to capture and represent complex relationships within the data. 
%One solution to overcome this issue is to use a set of basis functions and transfer from the $p\mhyphen$dimensional space to a higher-dimensional space, potentially capturing more complex relationships, say $q\mhyphen$dimensional space: $\bx\in \real^{p} \rightarrow \bphi(\bx) \in \real^{q}$ and $\bX \rightarrow \bPhi(\bX) \in \real^{n\times q}$. However, the introduction of basis functions comes with a computational challenge. 
%Specifically, as the dimensionality grows, the computational complexity increases quadratically. For instance, the computation of the predictive variance changes from $\mathcalO(p^2+p)$ to $\mathcalO(q^2+q)$.
%The quadratic growth in complexity can be a significant challenge, particularly when scalability and efficiency are important considerations.

One way to address the computation challenges is using the kernel trick.
\index{Kernel trick}
\begin{remark}[Kernel trick]
The kernel trick can significantly reduce computational costs in high-dimensional spaces. The key requirement is that all operations depend only on inner products between input vectors.
\end{remark}

\begin{example}[Kernel trick]
For some examples, we have:
\begin{enumerate}
\item   The computation of $\bphi(\bx_\ast)^\top \bPhi(\bX)^\top = \bz \in \real^{1\times n}$ involves only inner products in the input space. Therefore, we can apply the kernel trick. Each element of $\bz$ becomes $z_i = k(\bx_\ast, \bx_i)$, where $k(\cdot, \cdot)$ denotes  a kernel function. 
Computing the kernel function can be more efficient than directly working in the high-dimensional space, potentially requiring less than $\mathcalO(q)$ operations.
\item In contrast, the computation $\by^\top \bPhi(\bX)$ involves inner products between inputs and outputs, which cannot be expressed purely in terms of input-input dot products. Hence, the kernel trick does not apply here.
\end{enumerate}
In summary, we can apply the kernel trick only when all computations involve inner products between input vectors (i.e., input-input dot products).
\end{example}

With this in mind, if we can reformulate Equation~\eqref{equation:zeromean-predictive} so that all computations depend only on dot products in the input space, we can apply the kernel trick.
To achieve this, we map the original input features into a higher-dimensional space. This transformation replaces the original inputs \{$\bx_\ast, \bX$\} (the ``$x$-space" $\in \real^p$) with their corresponding feature representations \{$\bphi, \bPhi$\}   (the ``$z$-space" $\in \real^q$). Applying this mapping leads to the predictive distribution in the \textit{$z$-space form}:
\begin{equation}\label{equation:zeromean-predictive-high}
\boxed{
\begin{aligned}
z\mhyphen\mathrm{Space \, Form:}\\
g_\ast \mid  \bx_\ast \bX, \by, \sigma^2 
&\sim \normal(\bphi(\bx_\ast)^\top \bbeta_1, \bphi(\bx_\ast)^\top \bSigma_1 \bphi(\bx_\ast)) \\ 
&= \normal\left(\bphi(\bx_\ast)^\top \Big(\frac{1}{\sigma^2}\bPhi^\top\bPhi + \bSigma_0^{-1}\Big)^{-1}\Big(\frac{1}{\sigma^2}\bPhi^\top\by\Big), \bphi(\bx_\ast)^\top \bSigma_1 \bphi(\bx_\ast)\right), \nonumber
\end{aligned}
}
\end{equation}
where $\bSigma_1=(\frac{1}{\sigma^2}\bPhi^\top\bPhi + \bSigma_0^{-1})^{-1}$ (see Section~\ref{sec:bayesian-zero-mean}) and $\bSigma_0\in\real^{q\times q}$. Let us define $\bK \triangleq \bPhi \bSigma_0 \bPhi^\top$ and $\bphi_\ast \triangleq \bphi(\bx_\ast)$. 
Then, we can derive the identity:
\begin{equation}\label{equation:gp-equal}
\frac{1}{\sigma^2} \bPhi^\top(\bK +\sigma^2\bI) = \frac{1}{\sigma^2}\bPhi^\top(\bPhi \bSigma_0\bPhi^\top +\sigma^2 \bI) = \bSigma_1^{-1} \bSigma_0  \bPhi^\top, 
\end{equation}
where we use the matrix identity $\bA\bB\bA + \bA = \bA(\bB\bA +\bI) = (\bA\bB+\bI)\bA$. Note that we do not distinguish the notations of $\bSigma_0$ and $\bSigma_1$ in their notations in $x$ and $z$-spaces. 
However, we do distinguish the notations of $\bX$ in $x$-space and $\bPhi$ in $z$-space, respectively.

Next, by left-multiplying Equation~\eqref{equation:gp-equal} by $\bSigma_1$  and  right-multiplying by $(\bK +\sigma^2\bI)^{-1}$, we transform the predictive distribution into an expression involving only inner products about the inputs:
\begin{equation}\label{equation:zeromean-predictive-high-2}
\boxed{
\begin{aligned}
\mathrm{Inner\, Product\, Form:}\\
g_\ast \mid  \bx_\ast \bX, \by, \sigma^2 \sim \normal(&\bphi_\ast^\top \bSigma_0 \bPhi^\top(\bK+\sigma^2\bI)^{-1}\by, \\
&\bphi_\ast^\top \bSigma_0 \bphi_\ast-\bphiast^\top\bSigma_0 \bPhi^\top(\bK+\sigma^2\bI)^{-1} \bPhi\bSigma_0 \bphiast), \nonumber
\end{aligned}
}
\end{equation}
where we use the \textit{matrix inversion lemma} such that $(\bA+\bB\bC)^{-1} = \bA^{-1} -\bA^{-1} \bB(\bI+\bC\bA^{-1}\bB)^{-1}\bC\bA^{-1}$ if $\bA$ is nonsingular \citep{lu2021numerical}.
Now, all terms in the input space are expressed using the forms $\bphiast^\top\bSigma_0\bphiast, \, \bphiast^\top\bSigma_0 \bPhi^\top$, and $ \bPhi \bSigma_0 \bPhi$, which aligns with the requirements for applying the kernel trick.

We define kernel variables as follows: 
\begin{subequations}
\begin{align}
K(\bx_\ast, \bx_\ast)  &\triangleq  \bphi_\ast^\top \bSigma_0 \bphi_\ast; \\
K(\bx_\ast, \bX)      &\triangleq \bphi_\ast^\top \bSigma_0 \bPhi^\top;\\
K(\bX,\bX)            &\triangleq  \bPhi \bSigma_0 \bPhi^\top.
\end{align}
\end{subequations}
Using these definitions, we can now express the predictive distribution in kernel form:
\begin{equation}\label{equation:zeromean-predictive-high-kernel-form}
\boxed{
\begin{aligned}
\mathrm{Kernel \, Form: }\\
g_\ast \mid  \bx_\ast \bX, \by, \sigma^2 \sim \normal(&K(\bx_\ast, \bX)(K(\bX,\bX)+\sigma^2\bI)^{-1}\by, \\
&K(\bx_\ast, \bx_\ast)-K(\bx_\ast, \bX)(K(\bX, \bX)+\sigma^2\bI)^{-1} K(\bX, \bx_\ast)).
\end{aligned}
}
\end{equation}
Since $\bSigma_0$ is positive definite (see Section~\ref{sec:bayesian-zero-mean}), it can be factorized using the square root decomposition $\bSigma_0=\bSigma_0^{1/2}\bSigma_0^{1/2}$ (Theorem~\ref{theorem:unique-factor-pd}). 
Therefore, it can be absorbed into the kernel function. For example, let $\widetilde{\bphi}(\bx_\ast) = \bphi(\bx_\ast)\bSigma_0^{1/2}$, $K(\bx_\ast, \bx_\ast)$ can be denoted by $K(\bx_\ast, \bx_\ast) = \widetilde{\bphi}(\bx_\ast)^\top\widetilde{\bphi}(\bx_\ast)$.
Similarly, $K(\bx_\ast, \bX)$ can be represented as $K(\bx_\ast, \bX)=\widetilde{\bphi}_\ast^\top  \widetilde{\bPhi}^\top$, and $K(\bX,\bX)$ as $K(\bX,\bX)=\widetilde{\bPhi}  \widetilde{\bPhi}^\top$.

\section{Kernels in a Nutshell}

A kernel, employing basis functions, implicitly transforms the input vector $\bx \in \real^p$ into a higher-dimensional feature space $\bphi(\bx)\in \real^q$. 
This transformation changes the way we compute inner products: instead of computing $\bx^\top\bx'$ in the original $p$-dimensional space, we compute the kernel function $k(\bx, \bx')=\bphi(\bx)^\top\bphi(\bx')$ in the $q$-dimensional space. 
As a result, the kernel matrix $K(\bX, \bX)$---which contains all pairwise kernel evaluations between data points---has two important properties:
\begin{enumerate}
\item $K$ must be symmetric, i.e., $k(\bx, \bx') = k(\bx', \bx)$.
\item $K$ must be positive semidefinite (PSD).
\end{enumerate}

\index{Positive semidefinite}
\begin{proof}[of kernel matrix $K$ is PSD]
Let $k_{ij} \triangleq k(\bx_i, \bx_j)$, $\forall \, i, j \in \{1, 2, \ldots, n\}$. And let  $\bt\in\real^n$ be any real vector, we have 
\begin{equation}
\begin{aligned}
\bt^\top K \bt &= \sum_{i,j=1}^n t_i t_j k_{ij} =  \sum_{i,j=1}^n t_i t_j \bphi(\bx_i)^\top\bphi(\bx_j) \\
&=  \left(\sum_{i=1}^n t_i \bphi(\bx_i)\right)^\top \left(\sum_{j=1}^n t_j \bphi(\bx_j)\right) =  \norm{\sum_{i=1}^n t_i \bphi(\bx_i)}^2 \geq 0. \nonumber
\end{aligned}
\end{equation}
This completes the proof.
\end{proof}

At first glance, it may seem that $k(\bx, \bx')$ can be any arbitrary function of $\bx$ and $\bx'$. However, the requirement that the kernel matrix be positive semidefinite restricts the form of valid kernel functions. Specifically, this constraint ensures that every valid kernel corresponds to an implicit inner product in some (possibly infinite-dimensional) feature space.
\index{Gaussian kernel}
\index{Linear kernel}
\index{Polynomial kernel}
\begin{remark}[Kernels that are usually used: some specific kernels]\label{remark:kernrls}
The following are examples of widely used kernel functions:
\begin{enumerate}
\item  Linear kernel: $k(\bx, \bx') = \bx^\top\bx'$.
\item  Polynomial kernel: $k(\bx, \bx') = (\eta + \gamma \bx^\top\bx')^Q$ with $\gamma>0, \eta \geq 0$.
\item  Gaussian kernel: $k(\bx, \bx') = \exp(-\gamma \norm{\bx-\bx'}^2)$. 
We now show that the Gaussian kernel corresponds to an infinite-dimensional feature mapping. Without loss of generality, let $\gamma = 1$. Then,
\begin{equation}
\begin{aligned}
k(\bx, \bx') &= \exp\{- \norm{\bx-\bx'}^2\}
=\exp\{-\bx^\top\bx\}\exp\{-\bx'^\top\bx'\}\exp\{2\bx^\top\bx'\} \\
&\underset{\mathrm{expansion}}{\overset{\mathrm{Taylor}}{=}}
\exp\{-\bx^\top\bx\}\exp\{-\bx'^\top\bx'\}\exp\left\{\sum_{i=0}^\infty \frac{(2\bx^\top\bx')^i}{i!}\right\} \\
&=\sum_{i=0}^\infty\left(\exp\{-\bx^\top\bx\}\exp\{-\bx'^\top\bx'\} \sqrt{\frac{2^i}{i!}}\sqrt{\frac{2^i}{i!}} (\bx)^{i}\cdot (\bx')^i  \right) \\
&=\sum_{i=0}^\infty\left({\color{winestain}\exp\{-\bx^\top\bx\}\sqrt{\frac{2^i}{i!}} (\bx)^{i}} \cdot {\color{mylightbluetext}\exp\{-\bx'^\top\bx'\} \sqrt{\frac{2^i}{i!}}  (\bx')^i}  \right) \\
&={\color{winestain}\bphi(\bx)^\top} {\color{mylightbluetext}\bphi(\bx')},\nonumber
\end{aligned}
\end{equation}
where $\bphi(\bx) = \sum_{i=0}^\infty  \exp\{-\bx^\top\bx\}\sqrt{\frac{2^i}{i!}} (\bx)^{i} $. 
This shows that the Gaussian kernel maps inputs from a finite-dimensional space to an infinite-dimensional space. A similar derivation holds for general $\gamma > 0$.
\item  Other valid kernels: One major advantage of using kernel methods is that we do not need to explicitly define or compute the feature mapping $\bphi(\bx)$. Instead, we can work directly with the kernel matrix $K$, which encodes all necessary inner products in the feature space.
\end{enumerate}
\end{remark}

\index{Zero-mean prior}
\index{Gaussian processes}
\section{Gaussian Process from Zero-Mean Prior Model}
We use Gaussian processes (GPs) to model distributions over functions. GPs are a natural extension of multivariate Gaussian distributions to infinite index sets---either countably infinite or continuous. Formally, we define a Gaussian process as follows:
\begin{definition}[Gaussian process]
A Gaussian process is a collection of random variables, any finite number of which have a joint Gaussian distribution. The definition does not exclude Gaussian processes with finite index sets, which would be simply Gaussian distributions.
\end{definition}

\subsection{Noise-Free Observations}
Following the Bayesian linear model with a zero-mean prior introduced  in Section~\ref{sec:bayesian-zero-mean}, we assume a zero-mean prior on the weight coefficient $\bbeta \sim \normal(\bzero, \bSigma_0)$. 
For each input observation $\bx$, we define the output as  $g(\bx) \triangleq \bphi(\bx)^\top \bbeta$. 
The mean and covariance of the prior over function outputs are then given by:
\begin{subequations}\label{equation:noisefree-gp}
\begin{align}
\Exp[g(\bx)] &= \bphi(\bx)^\top \Exp[\bbeta] = 0, \\
\Exp[g(\bx) g(\bx')] &= \bphi(\bx)^\top \Exp[\bbeta \bbeta^\top] \bphi(\bx')=\bphi(\bx)^\top \bSigma_0 \bphi(\bx').
\end{align}
\end{subequations}
In this formulation, the prior covariance matrix is typically specified manually in the Bayesian linear model. Alternatively, it can be defined using a kernel function:
\begin{equation}
k(\bx, \bx')=\bphi(\bx)^\top \bSigma_0 \bphi(\bx').
\end{equation}

Now, suppose we have the training input design matrix $\bX$, training output vector $\by$, test input design matrix $\bXast$, and test output vector $\bgast$.
Using Equation~\eqref{equation:noisefree-gp}, we can write the joint distribution of the training outputs $\by$ and the test outputs $\bgast$ as:
$$
\left[
\begin{matrix}
\by  \\
\bgast 
\end{matrix}
\right] \sim \normal \left(\bzero, \left[
\begin{matrix}
K(\bX,\bX), & K(\bX,\bXast) \\
K(\bXast, \bX), &  K(\bXast, \bXast)
\end{matrix}
\right] \right).
$$
Given this joint distribution, we can apply standard properties of Gaussian distributions to derive the marginal distribution of the test outputs $\bgast$.
\begin{lemma}[Marginal distribution of test outputs]\label{lemma:noise-free-gaussian-identity}
Continuing from Section~\ref{sec:bayesian-zero-mean}, we assume a zero-mean prior on the weight coefficient $\bbeta \sim \normal(\bzero, \bSigma_0)$, and define $g(\bx) \triangleq \bphi(\bx)^\top \bbeta$. Given the observed training inputs $\bX$, the corresponding training outputs $\by$, and the test inputs $\bXast$, the marginal distribution of the test outputs $\bgast$ is:
\begin{equation}
\begin{aligned}
\bgast \mid  \bXast, \bX, \by \sim \normal(&K(\bXast, \bX)K(\bX,\bX)^{-1}\by, \\
&K(\bXast, \bXast)-K(\bXast, \bX)K(\bX, \bX)^{-1} K(\bX, \bXast)).\nonumber
\end{aligned}
\end{equation} 
\end{lemma}

\begin{proof}[of Lemma~\ref{lemma:noise-free-gaussian-identity}]
Let $\bx$ and $\by$ be jointly Gaussian random vectors:
$$
\left[
\begin{matrix}
\bx  \\
\by 
\end{matrix}
\right] \sim \normal \left(
\left[
\begin{matrix}
\bu_x \\
\bu_y
\end{matrix}
\right]
, 
\left[
\begin{matrix}
\bA, & \bC \\
\bC^\top, &  \bB
\end{matrix}
\right] \right).
$$
%\begin{mdframed}[hidealllines=\mdframehidelineNote,backgroundcolor=\mdframecolorSkip]
Using standard properties of multivariate Gaussian distributions,  the marginal distribution of $\bx$ and the conditional distribution of $\bx$ given $\by$ are 
\begin{subequations}
\begin{align}
\bx &\sim \normal(\bu_x, \bA);\\
\bx\mid \by &\sim \normal(\bu_x + \bC\bB^{-1}(\by-\bu_y) , \bA-\bC\bB^{-1}\bC^\top  ); \\
\by &\sim \normal(\bu_y, \bB); \\
\by\mid \bx &\sim \normal(\bu_y + \bC^\top\bA^{-1}(\bx-\bu_x),\bB-\bC^\top\bA^{-1}\bC). 
\end{align}
\end{subequations}
Applying these identities to our setting, where the joint distribution of $\by$ and $\bgast$ follows a multivariate Gaussian, we obtain:
\begin{equation}
\begin{aligned}
\bgast \mid  \bXast, \bX, \by \sim \normal(&K(\bXast, \bX)K(\bX,\bX)^{-1}\by, \\
&K(\bXast, \bXast)-K(\bXast, \bX)K(\bX, \bX)^{-1} K(\bX, \bXast)),\nonumber
\end{aligned}
\end{equation} 
which completes the proof.	
\end{proof}

The resulting marginal distribution of the test output $\bgast$ matches the kernel form of the Bayesian linear model with a zero-mean prior given in Equation~\eqref{equation:zeromean-predictive-high-kernel-form}. The only difference is the noise term and that now we are working with vector-valued outputs instead of scalar predictions (i.e., $g_\ast \rightarrow \bgast$).

\begin{remark}[What does a distribution over functions mean?]
When given a kernel defined as  $k(\bx, \bx')=\bphi(\bx)^\top \bSigma_0 \bphi(\bx')$, which corresponds to a zero-mean prior on the weight coefficient $\bbeta \sim \normal(\bzero, \bSigma_0)$, we obtain the following prior distribution over the test outputs $\bgast$:
$$
\begin{matrix}
\bgast 
\end{matrix}
\sim \normal \left(\bzero, \left[
\begin{matrix}
K(\bXast,\bXast)
\end{matrix}
\right] \right).
$$
Therefore, we can sample functions from $p(\bgast \mid  \bXast)$, where each sample represents a possible function evaluated at the test inputs $\bXast$. By sampling functions, we mean generating realizations of  output values corresponding to the given input points (possibly finite or infinite number of samples) and its distribution. This is precisely what is meant by a  \textit{distribution over functions}.

When observed data $\by$ and training inputs $\bX$ are available, we have the prior distribution over the training outputs $\by$ and test ouputs $\bgast$:
$$
\left[
\begin{matrix}
\by  \\
\bgast 
\end{matrix}
\right] \sim \normal \left(\bzero, \left[
\begin{matrix}
K(\bX,\bX), & K(\bX,\bXast) \\
K(\bXast, \bX), &  K(\bXast, \bXast)
\end{matrix}
\right] \right).
$$
From this, we can derive the posterior distribution of the test outputs given the observed data:
\begin{equation}
\begin{aligned}
\bgast \mid  \bXast \bX, \by \sim \normal(&K(\bXast, \bX)K(\bX,\bX)^{-1}\by, \\
&K(\bXast, \bx_\ast)-K(\bXast, \bX)K(\bX, \bX)^{-1} K(\bX, \bXast)). \nonumber
\end{aligned}
\end{equation} 
Using this posterior, we can sample functions from the conditional distribution $p(\bgast \mid  \bXast \bX, \by)$, which reflect our updated beliefs about the underlying function after observing the data.

%In a nutshell, $k(\bx, \bx')$ defines shape and prior knowledge about our problem.
\end{remark}

\index{Gaussian disturbance}
\subsection{Noisy Observations}
Following again from Section~\ref{sec:bayesian-zero-mean}, we assume a zero-mean prior $\bbeta \sim \normal(\bzero, \bSigma_0)$. And we define $g(\bx) \triangleq \bphi(\bx)^\top \bbeta +{\color{mylightbluetext} \epsilon}$ and ${\color{mylightbluetext} \epsilon \sim \normal(0, \sigma^2)}$. Then, the mean and covariance for the prior output are:
\begin{equation}
\begin{aligned}
\Exp[g(\bx)] &= \bphi(\bx)^\top \Exp[\bbeta] + \Exp[\bepsilon] = \bzero, \\
\Exp[g(\bx) g(\bx)] &= \bphi(\bx)^\top \Exp[\bbeta \bbeta^\top] \bphi(\bx)+ 2\bphi(\bx)^\top \Exp[\bbeta] \Exp[\epsilon]+\Exp[\epsilon^2] \\
&=\bphi(\bx)^\top \bSigma_0 \bphi(\bx) +{\color{mylightbluetext} \sigma^2}, \\ 
\Exp[g(\bx) g(\bx')] &= \bphi(\bx)^\top \Exp[\bbeta \bbeta^\top] \bphi(\bx')+ \bphi(\bx)^\top \Exp[\bbeta] \Exp[\epsilon']+\bphi(\bx')^\top \Exp[\bbeta] \Exp[\epsilon]+\Exp[\epsilon\epsilon']\\
&=\bphi(\bx)^\top \bSigma_0 \bphi(\bx'). \nonumber
\end{aligned}
\end{equation}
That is,
\begin{equation}
\cov(y_i, y_j) = k(\bx_i, \bx_j)\delta_{ij}, \nonumber
\end{equation}
where $\delta_{ij}$ is a Kronecker delta, which is equal to 1 if and only if  $i=j$ and 0 otherwise.
Therefore,  the joint distribution of the training outputs $\by$ and the test outputs $\bgast$ can be expressed as:
$$
\left[
\begin{matrix}
\by  \\
\bgast 
\end{matrix}
\right] \sim \normal \left(\bzero, \left[
\begin{matrix}
K(\bX,\bX) + {\color{mylightbluetext} \sigma^2\bI}, & K(\bX,\bXast) \\
K(\bXast, \bX), &  K(\bXast, \bXast)
\end{matrix}
\right] \right).
$$
Using Gaussian identities, we can derive the marginal distribution of the test outputs $\bgast$.
\index{Marginal distribution}
\begin{lemma}[Marginal distribution of test outputs]\label{lemma:noisy-gaussian-identity}
Following from Section~\ref{sec:bayesian-zero-mean}, assume a zero-mean prior $\bbeta \sim \normal(\bzero, \bSigma_0)$ and define $g(\bx) \triangleq \bphi(\bx)^\top \bbeta+ {\color{mylightbluetext} \epsilon}$, where ${\color{mylightbluetext} \epsilon \sim \normal(0, \sigma^2)}$. Given observed training inputs $\bX$, the corresponding training outputs $\by$, and the test inputs $\bXast$, the marginal distribution of the test outputs $\bgast$ is
\begin{equation}
\begin{aligned}
\bgast \mid  \bXast \bX, \by \sim \normal(&K(\bXast, \bX)(K(\bX,\bX)+{\color{mylightbluetext} \sigma^2\bI})^{-1}\by, \\
&K(\bXast, \bXast)-K(\bXast, \bX)(K(\bX, \bX)+{\color{mylightbluetext} \sigma^2\bI})^{-1} K(\bX, \bXast)). \nonumber
\end{aligned}
\end{equation} 
\end{lemma}
The marginal distribution of the test outputs $\bgast$ aligns with the form of Equation~\eqref{equation:zeromean-predictive-high-kernel-form}, except that now the output is a vector (i.e., $g_\ast \rightarrow \bgast$).

\subsection{Further Extension, Generalized Gaussian Process} 
Building  upon the concept   of the generalized least squares introduced in Section~\ref{section:generalizedLS}, 
we can now assign different noise variances to individual observations. Specifically, suppose the noise covariance matrix takes the form $\sigma^2\bLambda$, where $\bLambda$ is a diagonal matrix. This allows each observation to have its own noise level, which is particularly useful when dealing with heteroscedastic (non-constant variance) data.
Under this assumption, the marginal distribution of the test outputs $\bgast$ becomes:
\begin{equation}
\begin{aligned}
\bgast \mid  \bXast \bX, \by \sim \normal(&K(\bXast, \bX)(K(\bX,\bX)+{\color{mylightbluetext} \sigma^2\bLambda})^{-1}\by, \\
&K(\bXast, \bXast)-K(\bXast, \bX)(K(\bX, \bX)+{\color{mylightbluetext} \sigma^2\bLambda})^{-1} K(\bX, \bXast)). \nonumber
\end{aligned}
\end{equation} 
It's important to note that the noise covariance should not depend on the input matrix $\bX$. Otherwise, the kernel trick---which allows us to work implicitly in high-dimensional feature spaces---would no longer be applicable.

\begin{figure}[h]
\centering  
\vspace{-0.35cm}  
\subfigtopskip=2pt 
\subfigbottomskip=2pt  
\subfigcapskip=-5pt 
\subfigure[Use half data set.]{\label{fig:gphalf}
\includegraphics[width=0.475\linewidth]{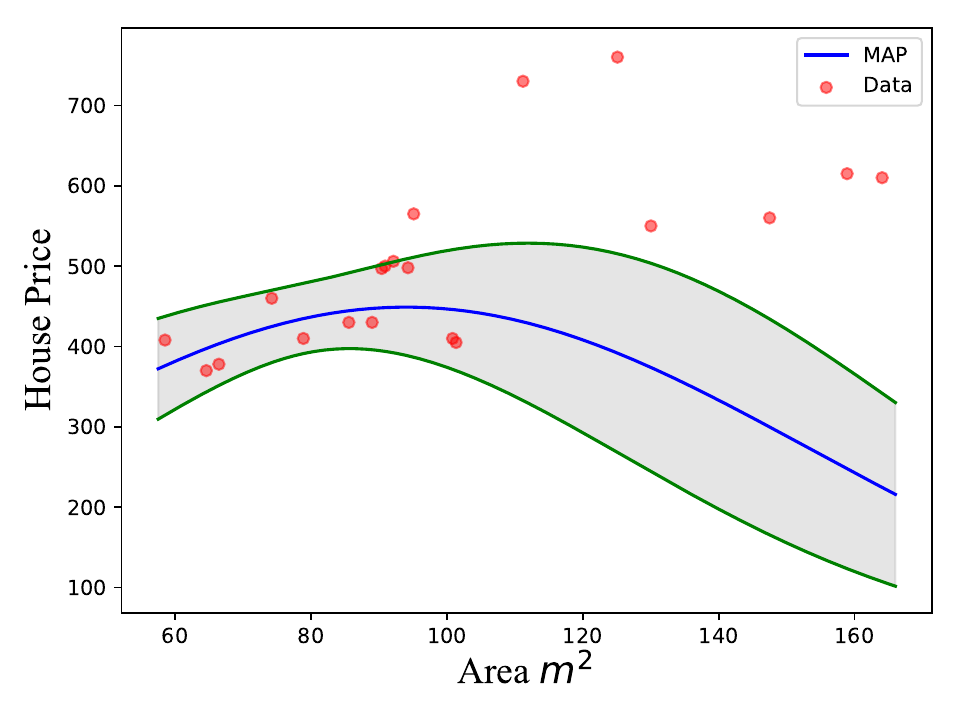}}
\quad 
\subfigure[Use entire data set.]{\label{fig:gpall}
\includegraphics[width=0.475\linewidth]{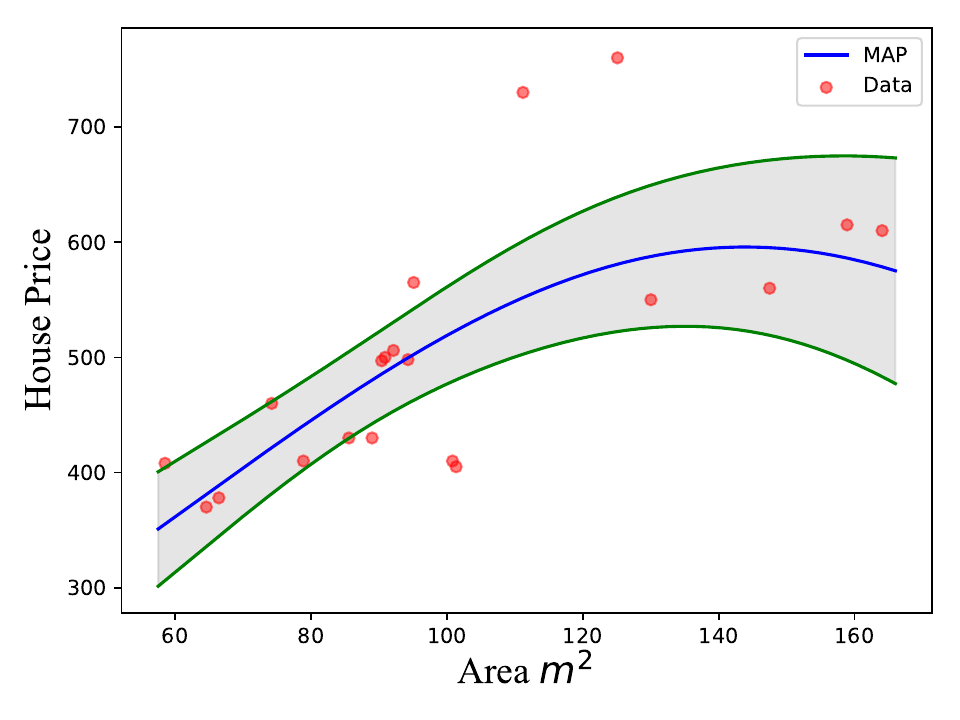}}
\caption{Comparison of using half of the data versus the full data set.}
\label{fig:gphalfandall}
\end{figure}

\begin{figure}[h!]
\centering  
\vspace{-0.35cm} 
\subfigtopskip=2pt 
\subfigbottomskip=2pt 
\subfigcapskip=-5pt 
\subfigure[Realization 1.]{\label{fig:gphalfrd1}
\includegraphics[width=0.475\linewidth]{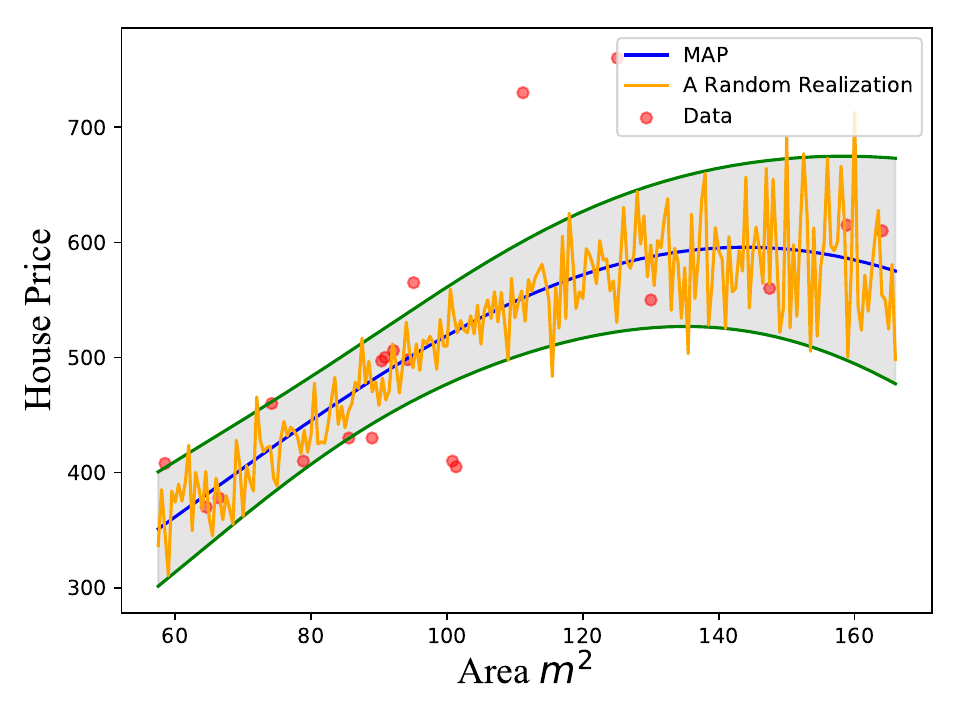}}
\quad 
\subfigure[Realization 2.]{\label{fig:gpallrd2}
\includegraphics[width=0.475\linewidth]{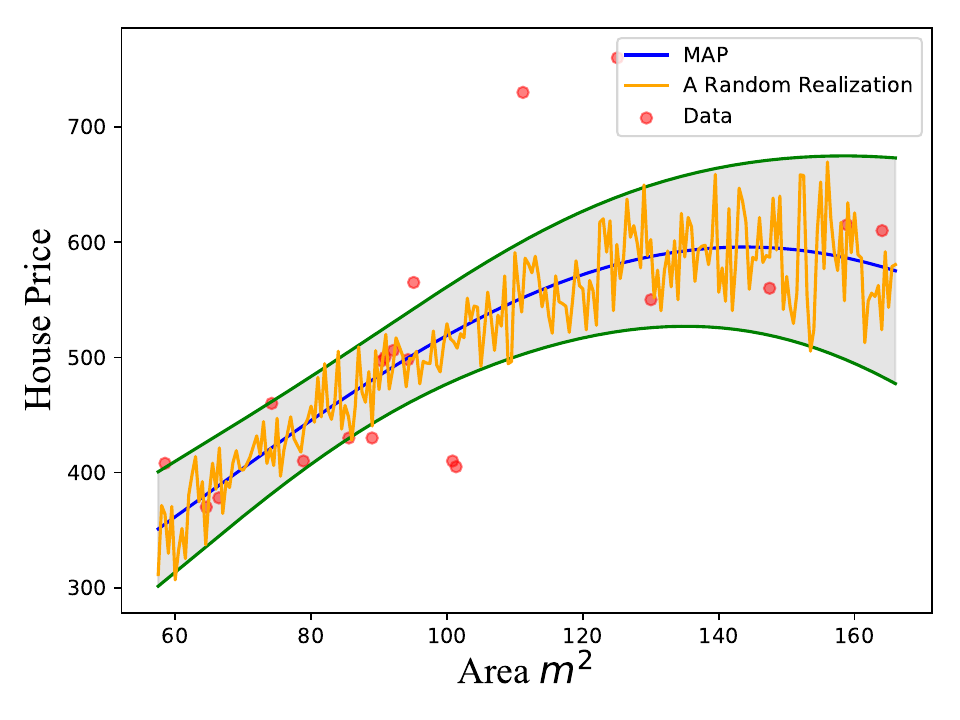}}
\quad
\subfigure[Realization 3.]{\label{fig:gphalfrd3}
\includegraphics[width=0.475\linewidth]{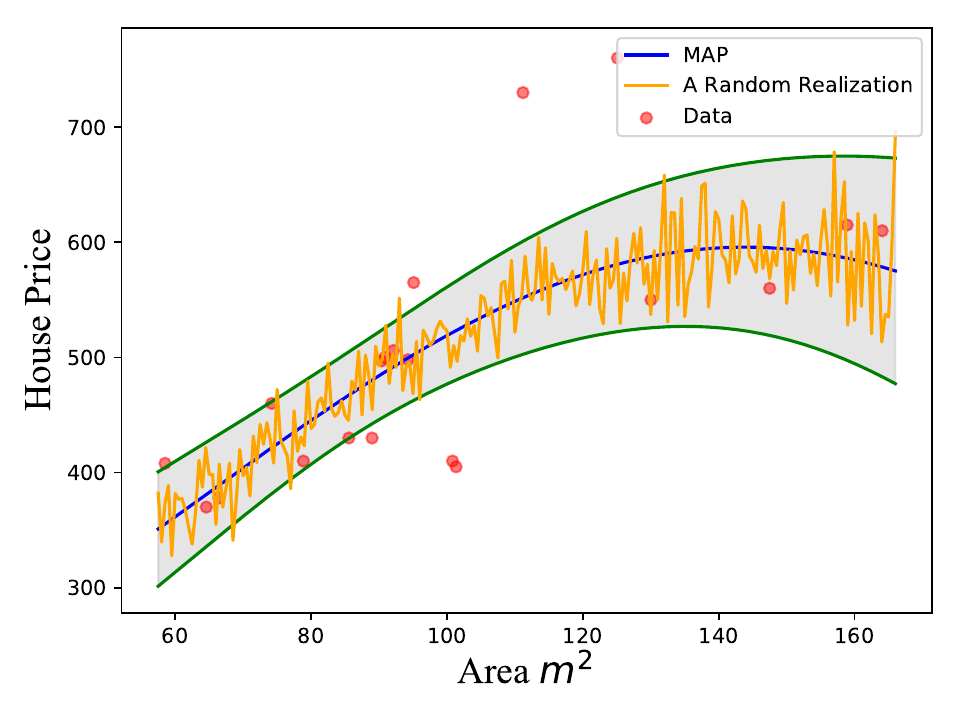}}
\quad 
\subfigure[Realization 4.]{\label{fig:gpallrd4}
\includegraphics[width=0.475\linewidth]{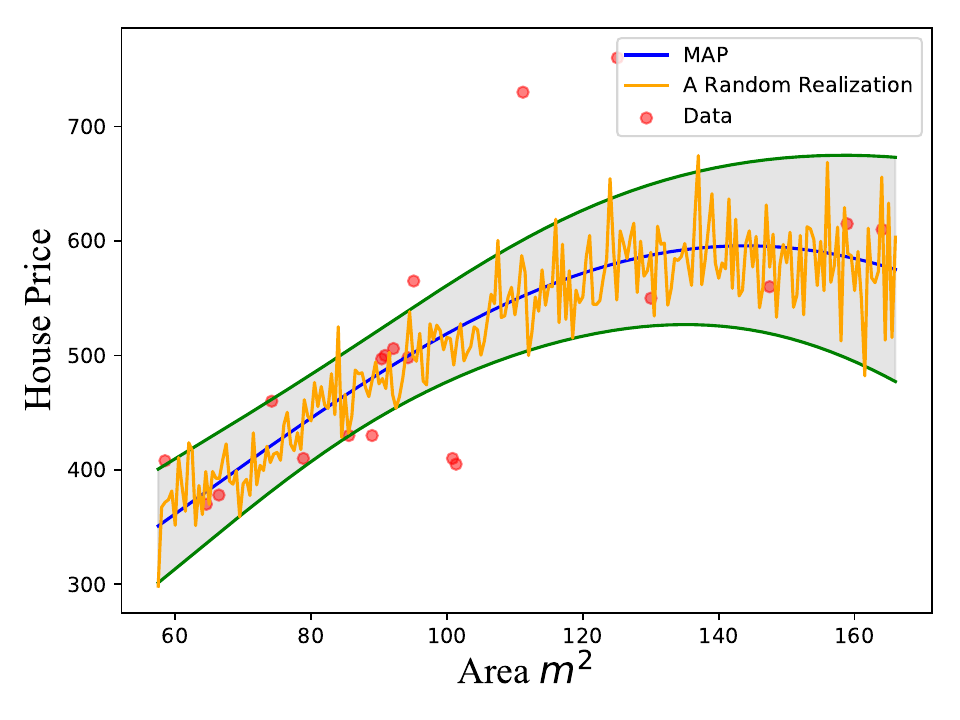}}
\caption{Random realizations from the posterior distribution.}
\label{fig:gphalfandall-rds}
\end{figure}
\begin{figure}[h!]
\centering  
\vspace{-0.35cm} 
\subfigtopskip=2pt  
\subfigbottomskip=2pt  
\subfigcapskip=-5pt  
\subfigure[Repeat the data set once.]{\label{fig:gp_rept1}
\includegraphics[width=0.475\linewidth]{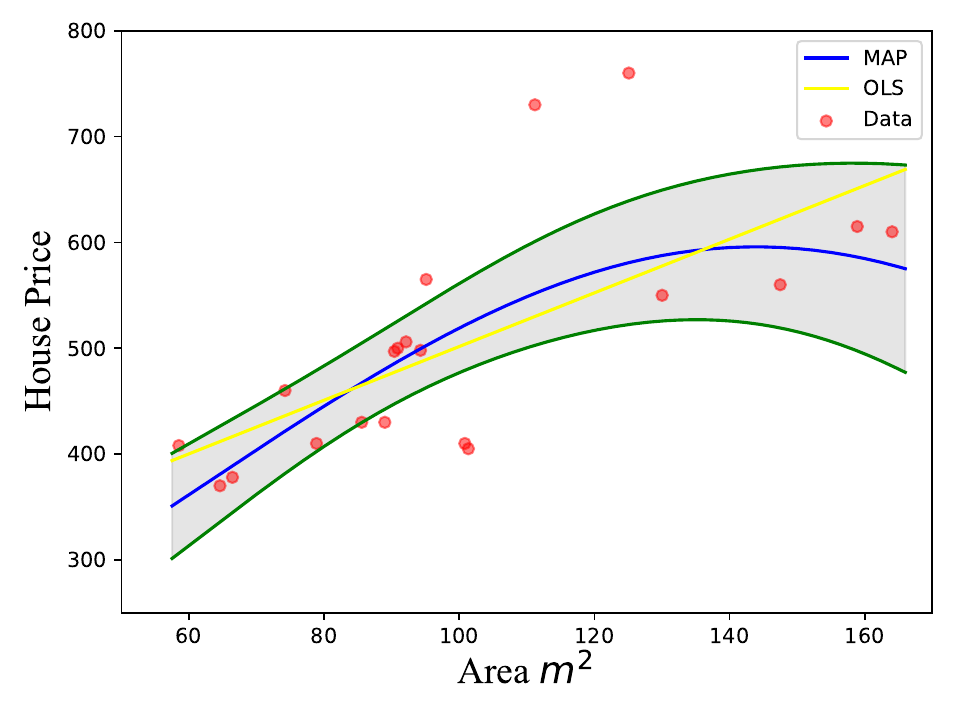}}
\quad 
\subfigure[Repeat the data set twice.]{\label{fig:gp_rept2}
\includegraphics[width=0.475\linewidth]{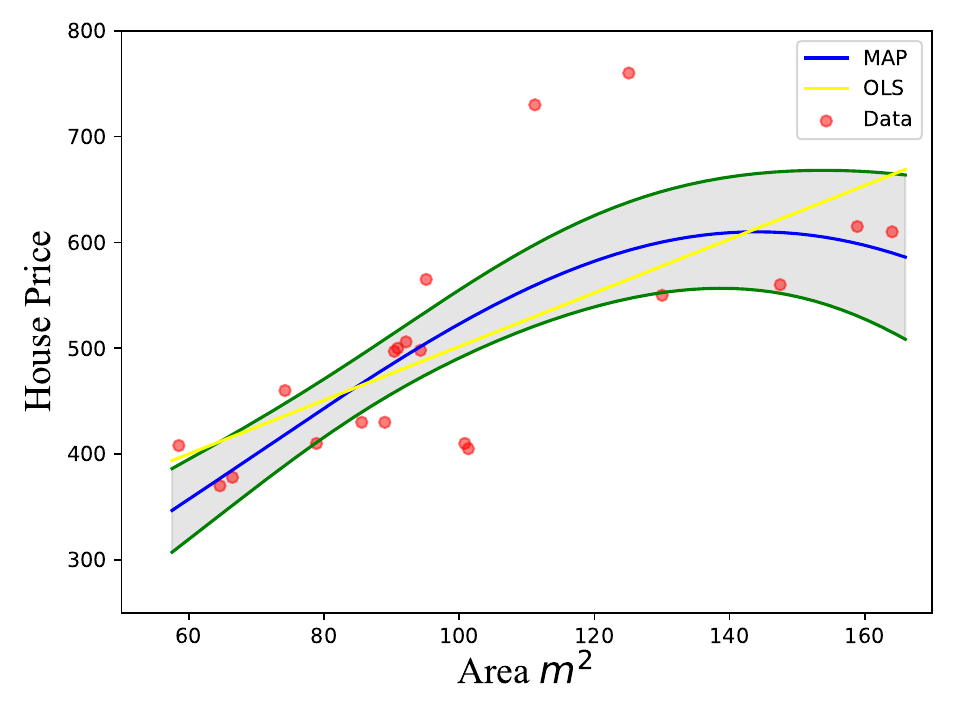}}
\quad
\subfigure[Repeat the data set three times.]{\label{fig:gp_rept3}
\includegraphics[width=0.475\linewidth]{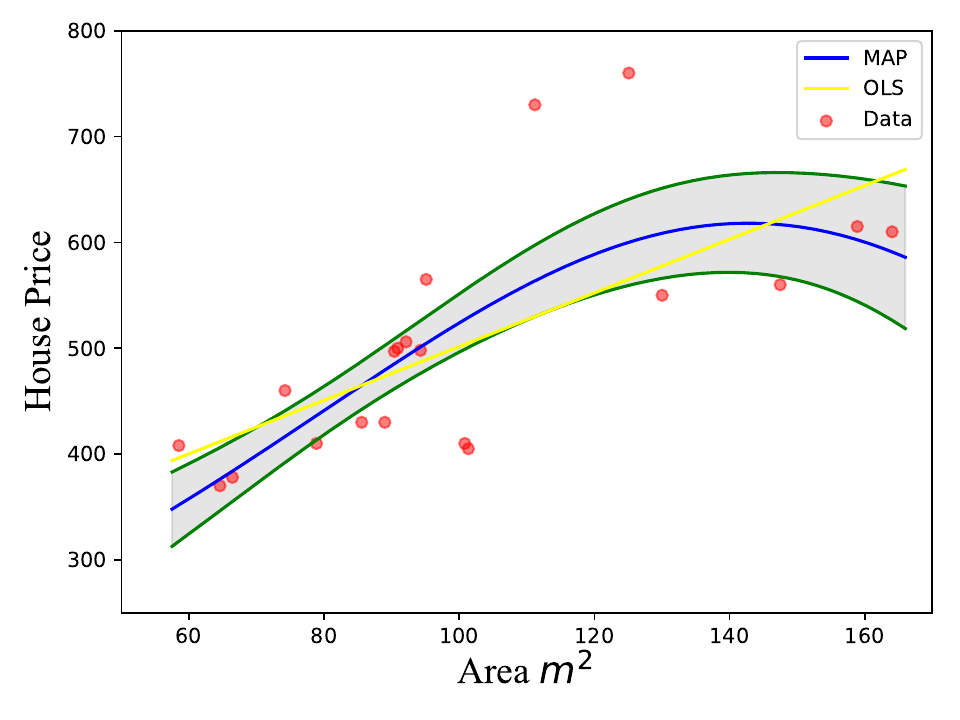}}
\quad 
\subfigure[Repeat the data set four times.]{\label{fig:gp_rept4}
\includegraphics[width=0.475\linewidth]{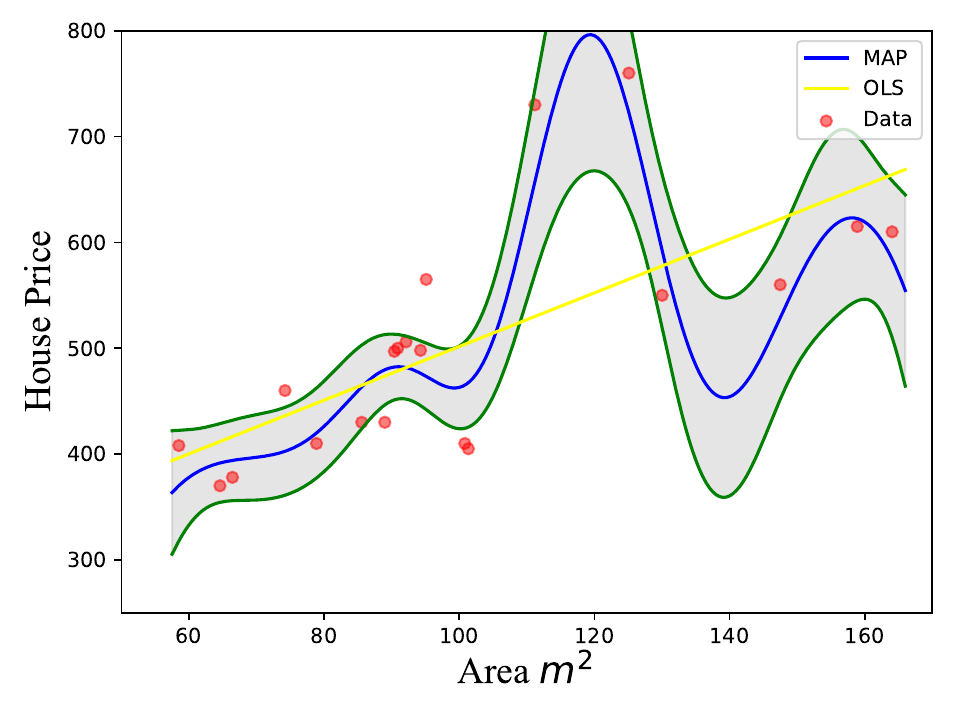}}
\caption{Effect of repeating the data set multiple times.}
\label{fig:gphalfandall-repeat}
\end{figure}
\begin{example}
We consider a dataset where the input variable represents the area of a house, and the output variable corresponds to its monthly rent. To determine the optimal parameters of the Gaussian kernel, we use cross-validation (CV).

In Figure~\ref{fig:gphalfandall}, the red dots represent the training inputs, the blue line signifies the MAP estimate via a GP regressor, and the shaded region indicates the  95\% confidence interval.

In Figure~\ref{fig:gpall}, the GP regressor is trained using the full dataset. 
In contrast, Figure~\ref{fig:gphalf} shows results when only data points corresponding to houses smaller than 100 $m^2$ are used for training.
As expected, the model performs poorly when extrapolating beyond 100 $m^2$, indicating that the estimator trained on limited data cannot generalize well to unseen regions of the input space.

Figure~\ref{fig:gphalfandall-rds} displays four random function realizations (shown by orange lines) drawn from the posterior distribution using the full dataset. These realizations help illustrate the concept of a distribution over functions, which is central to understanding Gaussian processes.

In Figure~\ref{fig:gphalfandall-repeat}, we investigate how increasing the size of the dataset by repetition affects the predictions. Specifically, we repeat the dataset once, twice, three times, and four times. The yellow line in each plot represents the prediction made by ordinary least squares (OLS) regression.
Interestingly, the OLS estimate remains unchanged regardless of how many times the data is repeated. To see why, suppose the original OLS estimate is given by: 
$$
\widehatbbeta_1 = (\bX^\top \bX)^{-1}\bX^\top\by.
$$
Now, if we repeat the data twice, the new estimate becomes:
$$
\begin{aligned}
\widehatbbeta_2 &= \left(\begin{bmatrix}
\bX \\
\bX
\end{bmatrix}^\top \begin{bmatrix}
\bX \\
\bX
\end{bmatrix}\right)^{-1}\begin{bmatrix}
\bX \\
\bX
\end{bmatrix}^\top\begin{bmatrix}
\by \\
\by
\end{bmatrix} 
= \frac{1}{2}(\bX^\top \bX)^{-1} \cdot 2\bX^\top\by=\widehatbbeta_1.
\end{aligned}
$$
This result confirms that repeating the data does not change the OLS estimate. Similarly, repeating the data three or four times also yields the same result. Therefore, OLS treats repeated observations as redundant information.

However, this is not the case for Gaussian process regression. When the dataset is repeated, the GP estimator becomes more confident in its predictions, reflected by a reduction in predictive variance. This behavior highlights a key advantage of Bayesian models: they update their beliefs as more data becomes available.

As discussed in Section~\ref{sec:beta-bernoulli}, Bayesian models incorporate prior knowledge about the parameters, making them particularly effective for regularizing regression problems when data is scarce. The amount of observed data plays a critical role in shaping the posterior distribution, as illustrated in Example~\ref{example:amountofdata}. This example clearly demonstrates the Bayesian foundation of Gaussian process modeling.
\end{example}

\section{Limitations of Gaussian Process from Non-Zero-Mean Prior*}
When the prior mean $\bbeta_0$ is not zero (see Equation~\eqref{equation:semiconju_blm}), and the model assumes Gaussian noise $\bepsilon\sim \normal(\bzero, \sigma^2\bI)$, the predictive distribution becomes:
\begin{equation}\label{equation:non-zeromean-predictive-high}
\begin{aligned}
g_\ast \mid  \bx_\ast \bX, \by, \sigma^2 \sim \normal\Big(&\bphi_\ast^\top \bSigma_0 \bPhi^\top(\bK+\sigma^2\bI)^{-1}\by + {\color{mylightbluetext} \bphiast^\top\Big(\frac{1}{\sigma^2} \bSigma_0 \bPhi^\top\bPhi  + \bI\Big)^{-1}\bbeta_0}, \\
&\bphi_\ast^\top \bSigma_0 \bphi_\ast-\bphiast^\top\bSigma_0 \bPhi^\top(\bK+\sigma^2\bI)^{-1} \bPhi\bSigma_0 \bphiast\Big), \nonumber
\end{aligned}
\end{equation}
where we use the fact that $(\bA\bB)^{-1} = \bB^{-1} \bA^{-1}$ if both $\bA$ and $\bB$ are nonsingular.

Following from Section~\ref{sec:bayesian-zero-mean}, suppose we now assume a non-zero-mean prior $\bbeta \sim \normal(\bbeta_0, \bSigma_0)$ and define $g(\bx) \triangleq \bphi(\bx)^\top \bbeta +\epsilon$ with $\epsilon\sim\normal(0,\sigma^2)$. 
Then, the mean and covariance of the output under this prior are given by:
\begin{equation}
\begin{aligned}
\Exp[g(\bx)] &= \bphi(\bx)^\top \Exp[\bbeta] = \textcolor{mylightbluetext}{\bphi(\bx)^\top \bbeta_0}, \\
\Exp[g(\bx) g(\bx')] &= \bphi(\bx)^\top \Exp[\bbeta \bbeta^\top] \bphi(\bx')=\bphi(\bx)^\top \bSigma_0%+\textcolor{mylightbluetext}{\bbeta_0\bbeta_0^\top}
 \bphi(\bx'). \nonumber
\end{aligned}
\end{equation}
Based on this, the joint distribution of the training outputs $\by$ and test outputs $\bgast$ can be written as:
$$
\left[
\begin{matrix}
\by  \\
\bgast 
\end{matrix}
\right] \sim \normal 
\left(
\left[
\begin{matrix}
{\color{mylightbluetext}\phi(\bx)^\top\bbeta_0} \\
{\color{mylightbluetext}\phi(\bgast)^\top\bbeta_0}
\end{matrix}
\right] , 
\left[
\begin{matrix}
K(\bX,\bX) +  \sigma^2\bI, & K(\bX,\bXast) \\
K(\bXast, \bX), &  K(\bXast, \bXast)
\end{matrix}
\right] 
\right).
$$
By Gaussian identities, the marginal distribution of test outputs is
$$
\begin{aligned}
\bgast \mid  \bXast \bX, \by \sim \normal(& {\color{mylightbluetext} \phi(\bgast)^\top\bbeta_0}+ K(\bXast, \bX)(K(\bX,\bX)+\sigma^2\bI)^{-1}(\by - {\color{mylightbluetext} \phi(\bx)^\top\bbeta_0  }), \\
&K(\bXast, \bx_\ast)-K(\bXast, \bX)(K(\bX, \bX)+\sigma^2\bI)^{-1} K(\bX, \bXast)). \nonumber
\end{aligned}
$$
Note that in this formulation, we cannot express the mean of $g(\bx)$ explicitly in terms of kernel evaluations alone. That is, there is no equivalent expression in the $z$-dimensional feature space that allows us to apply the kernel trick directly.

This limitation restricts the flexibility of models using non-zero-mean priors, especially in high-dimensional or kernel-defined spaces where working explicitly with features is infeasible. As a result, zero-mean priors are often preferred in practice for their compatibility with kernel methods and computational efficiency.

\begin{problemset}
%% below are prob for Bayes
\item \label{problem:proptionto} Suppose both $p(x)$ and $q(x)$ are probability density functions with 
$$
p(x) \propto q(x).
$$ 
Show that $p(x) = q(x)$ for all $x$.

\item \textbf{Dirichlet-Multinomial.} Following the Beta-Bernoulli model (Section~\ref{sec:beta-bernoulli}), show that the Dirichlet distribution is a conjugate prior for the multinomial distribution.

\item We have shown that the Gamma distribution is a conjugate prior for the precision parameter of a Gaussian distribution. 
Show that the inverse-Gamma distribution (Definition~\ref{definition:inverse_gamma_distribution}) is a conjugate prior for the variance parameter of a Gaussian distribution with a fixed mean parameter.

% https://revbayes.github.io/tutorials/intro/graph_models.html
\item Replace the Gaussian noise in the Bayesian linear model using a zero-mean prior with an exponential noise. Discuss how to construct a Gibbs sampler for this modified model and draw its graphical representation. See also the exponential MLE derivation in Problem~\ref{problem:exponential_mle}.

\item \label{prob:equiv_invg_ga_var} Prove the equivalence between the inverse-Gamma density on the variance parameter and the Gamma density on the prevision parameter of a Gaussian distribution. \textit{Hing: use the Jacobian in the change-of-variables formula.}

\item Following Section~\ref{section:niv_dist}, show that an equivalent conjugate prior on the mean and precision parameters of a Gaussian distribution is the normal-Gamma distribution. Define the normal-Gamma distribution yourself before proceeding with the proof.

\item \textbf{Poisson and conjugacy.} Let $\rx_1, \rx_2, \ldots, \rx_N$ be i.i.d. random variables drawn from the Poisson distribution $\poissondist(\lambda)$. Suppose the prior for $\lambda$ is 
$$
\gammadist(\lambda \mid a, b) = \frac{b^a}{\Gamma(a)} \lambda^{a-1} \exp(-b \lambda) \indicator(\lambda >0).
$$
Derive the posterior distribution of $\lambda$ after observing the data.

%% below are prob for GP
\item Derive \eqref{equation:zeromean-predictive} rigorously.
\item Show that the linear kernel, polynomial kernel, and Gaussian kernel introduced in Remark~\ref{remark:kernrls} satisfy  the positive definiteness property.
\end{problemset}

\newpage
\chapter{Generalized Linear Model (GLM)}\label{chapter:glm}
\begingroup
\hypersetup{
linkcolor=structurecolor,
linktoc=page,  % page: only the page will be colored; section, all, none etc
}
\minitoc \newpage
\endgroup

%\section{A Motivating Example}
\lettrine{\color{caligraphcolor}I}
In the Gauss-Markov model (Chapter~\ref{sec:lr-gaussian-noise}), for each input observation $(\bx, y)=(\bx_i, y_i)$, $i=1,2,\ldots,n$ with $\bx\in\real^p$, we assume the linear relationship $y=f(\bx)\triangleq \bx^\top\bbeta$.
This relationship is known as the  \textit{linear predictor}. The linear predictor serves as the mean of a Gaussian distribution, meaning the output variable can be described as:
\begin{equation}\label{equation:func_rel}
\Exp[\ry \mid f, \bx] = f(\bx)
\qquad \implies\qquad 
\ry\sim \normal(f(\bx), \sigma^2).
\end{equation}~\footnote{Note again that we use normal fonts of boldface lowercase letters to denote random vectors, and  normal fonts of boldface uppercase letters to denote random matrices. That is, $\rx, \rva, \rmX$ are random scalars, vectors, or matrices; while $x, \ba, \bX$ are scalars, vectors, or matrices. In many cases, the two terms can be used interchangeably; that is, $\rx=x$ denotes a realization of the variable.}
The relationship in \eqref{equation:func_rel} is called the \textit{functional relationship} between predictors and the response, indicating that $p(y\mid \bx, f) = p(y\mid f(\bx))$, with $\Exp[\ry\mid \bx, f] = f(\bx)$, for some function $f \in \mathcalF$ within a suitable set of possible functions.

The Gauss-Markov setup is well-suited for regression problems where the output values are continuous. 
However, when dealing with discrete, binary, or categorical output values, the assumptions of the Gauss-Markov model become less practical and may not hold true.
For instance, consider scenarios where the output represents categories (e.g., types of animals), binary outcomes (e.g., success or failure), or positive values (e.g., individual income or monthly house rent). In these cases, modeling the output as a continuous Gaussian variable does not accurately reflect the nature of the data. The primary limitation arises from the fact that the Gaussian distribution assumes a continuous range of possible outcomes, which is inappropriate for discrete or binary outputs.

To address these limitations, \textit{generalized linear models (GLMs)} offer a more flexible framework \citep{nelder1972generalized, dobson2018introduction, dunn2018generalized, mccullagh2019generalized, jackson2024glm}. GLMs extend the concept of the linear predictor to accommodate various types of response variables by linking the expected value of the response to the linear predictor through a link function. This allows the model to handle different types of distributions from the exponential family, such as binomial for binary outcomes or Poisson for count data.

\index{Heteroskedasticity}
\index{Homoskedasticity}
The relationship between the expected value of the response variable and the linear predictor is modeled via a link function, which can be chosen based on the nature of the response variable.
The variance of the response variable can depend on its mean, allowing for \textit{heteroskedasticity}~\footnote{\textit{Homoskedasticity} refers to a condition in statistics where the variance of the error terms or residuals in a regression model is constant across all levels of the independent variables. When homoskedasticity is present, it implies that the prediction errors do not systematically increase or decrease with changes in the value of the independent variable(s). However, when this condition is violated, and the variance of the residuals changes at different levels of an independent variable, the data are said to exhibit \textit{heteroskedasticity}.}, which is common in non-Gaussian settings.
By incorporating these features, GLMs provide a robust approach for modeling a wide array of data types, including binary, count, and categorical data. This makes them invaluable tools in fields ranging from healthcare (for predicting disease outcomes) to economics (for analyzing consumer behavior), where the responses often do not follow a normal distribution \citep{powers2005predictive, egger2016glm}.

Thus, while the Gauss-Markov linear model provides a solid foundation for understanding relationships in continuous data, GLMs extend this capability to encompass a broader spectrum of applications and data types, enhancing the model's flexibility and applicability across diverse domains.

\index{Logistic regression}
\index{Generalized linear models}
\section{A Motivating Example: Logistic Regression}\label{section:logis_reg}

In this section, we will explore a specific case that is important in its own right and serves as an introduction to the fundamental concepts of GLMs, which we will discuss in greater detail later.
For a dataset in which the response variable $\ry$ is binary, taking values 1 or 0 to represent success or failure (presence or absence), 
the expectation of $\ry$ must fall within the interval $[0,1]$:
\begin{equation}
	\Exp[\ry\mid f,\bx] = f(\bx) \in [0,1].
\end{equation}
Given this requirement, the Bernoulli distribution (Equation~\eqref{equation:bbern_dist}) is the appropriate probability distribution, characterized solely by the probability $\pi$ that $\ry=1$.
We assume a particular distributional form for $y$ (in this case Bernoulli), and would like to connect $\bx$ with a feature of this distribution, namely its expectation, by a function $f$:
\begin{equation}
p(\ry=1\mid f,\bx) =  \Exp[\ry\mid f,\bx] = f(\bx).
\end{equation}
Thus, with the spaces $\mathcal{Y}$ for $\ry$ and $\mathcal{X}$ for $\bx$ established, along with a probability distribution satisfying \eqref{equation:func_rel}, our task is to identify a suitable set of possible functions, $\mathcalF$.

A straightforward approach might be to apply a linear function:
\begin{equation}
	f(\bx) = \bbeta^\top \bx,
\end{equation}
However, such a function may not always be suitable because, for any given value of $\bbeta$, there could be values of $\bx$ that make $\bbeta^\top \bx$ fall outside the interval $[0,1]$. This issue can sometimes be avoided if $\bx$ takes on only certain ranges of values, but this cannot be guaranteed universally. To address this, GLMs introduce a \textit{response function} applied to the linear predictor $\eta = \bbeta^\top \bx$, ensuring it falls within the correct range. Specifically, we choose a function $h : \real \to [0,1]$, setting
\begin{equation}
\Exp[\ry\mid f,\bx] = f(\bx) = h(\bbeta^\top \bx).
\end{equation}
Such a function in the logistic regression is called the \textit{logistic function} (a.k.a., the \textit{sigmoid function}):
$$
h(\eta) \triangleq \text{Sigmoid}(\eta) =  \frac{e^{\eta}}{1+e^{\eta}} = \frac{1}{1+e^{-\eta}}.
$$
Its inverse $h^{-1}(\pi) = \ln\left(\frac{\pi}{1-\pi}\right)$, where $\pi\in(0,1)$, is known as the \textit{logit function}.

\index{Logistic function}
\index{Sigmoid function}
\paragrapharrow{Logit and probit.}

Alternatively, one can use the cumulative distribution function (or any other cumulative distribution function) of the standard Gaussian distribution as the response function:
\begin{align}
f(\bx) &= \Phi(\bbeta^\top \bx)
\qquad \text{and}\qquad 
\Phi^{-1}(f(\bx)) = \bbeta^\top \bx,
\end{align}
where $ \Phi(y) = \int_{-\infty}^{y} \normal(u\mid 0,1)du= \frac{1}{\sqrt{2\pi}} \int_{-\infty}^{y} \exp(-\frac{u^2}{2}) du $ is the cumulative distribution function of a standard Gaussian distribution.
Its inverse, $\Phi^{-1}$, is known as the \textit{probit function}. 
While logit and probit estimators differ when estimated probabilities are very small or close to 1, indicating large sample sizes are needed for accurate inference, both functions yield similar results, especially around probability values of 0.5. For a comparison between the logit and probit functions, see Table~\ref{tab:logit-probit-comparison}.

\begin{table}[h]
\centering
\begin{tabular}{|c|c|c|}
\hline
& Logit & Probit \\
\hline\hline
Response function $ h $ & logistic function & Gaussian cdf \\
\hline
Model & $f(\bx) = \frac{e^{\bbeta^\top \bx}}{1 + e^{\bbeta^\top \bx}}$ & $f(\bx) = \Phi(\bbeta^\top \bx)$ \\
\hline
Inverse & $\text{logit}(f(\bx)) = \bbeta^\top \bx$ & $\Phi^{-1}(f(\bx)) = \bbeta^\top \bx$ \\
\hline
\end{tabular}
\caption{A comparison of the logit and probit models.}
\label{tab:logit-probit-comparison}
\end{table}
\paragrapharrow{Model components.}
As outlined above, the logistic regression model comprises three components:
\begin{itemize}
\item The \textit{linear predictor}: $\eta = \bbeta^\top \bx$ (same as that in the Gauss-Markov model).
\item The  \textit{response function} (here, we use the logistic function):
$\Exp[\ry\mid f,\bx] = f(\bx) = f(\bx) = h(\eta) = \frac{e^\eta}{1 + e^\eta}$ (an identity function in the Gauss-Markov model).
\item The \textit{probability distribution}: $\ry \sim \bernoullidist(f(\bx))$ (a normal distribution in the Gauss-Markov model).
\end{itemize}
The goal of the logistic regression, which is the same as the Gauss-Markov linear model, is to say something about the distribution of $\ry$ at $\bx \in \mathcalX$ for which the response $\ry$ is unobserved.
The estimation of this model lies in obtaining an estimate of $\bbeta$ given a set of data samples.
Common methods for estimating the logistic regression is gradient descent methods, which will not be discussed here; see, for example, \citet{lu2025practical}. 
General estimation methods for GLMs are covered in Section~\ref{section:estima_glms}.

\section{Exponential Families of Distributions and Variants}

Though it may not be immediately obvious at first sight, many of the distributions we discussed earlier---whether discrete or continuous---share some important similarities in terms of their structure and their properties. 
To take advantage of these commonalities, we introduce in this section an additional level of abstraction by considering most of these distributions as special cases of a broader class of probability models known as the \textit{exponential family (EF) of distributions}. 
This approach has a key benefit: once we establish general properties for the exponential family, they will automatically apply to all its specific instances. 

\subsection{Exponential Families of Distributions}
We now define what is meant by an exponential family of distributions:
\begin{subequations}
\begin{definition}[The exponential family (EF) of distributions]\index{Exponential family}
A regular probability distribution is said to be a member of a \textit{$k$-parameter exponential family}, if its probability density (or mass) function can be written in the following form:
\begin{equation}\label{equation:exp_fam_dist}
\begin{aligned}
p(\by \mid  \bphi ) 
&= \exp \left\{ \bphi^\top \bT(\by) - \Gamma(\bphi) + S(\by) \right\}\\
&= \exp \left\{ \sum_{i=1}^k \phi_i T_i(\by) - \Gamma(\phi_1, \ldots, \phi_k) + S(\by) \right\}, \qquad \by \in \mathcalY,
\end{aligned}
\end{equation}
where:
\begin{enumerate}
\item $\bphi = [\phi_1, \ldots, \phi_k]^\top$ is a $k$-dimensional vector of parameters in $\real^k$, referred to as the \textit{natural parameter} or \textit{canonical parameter}.
\item $\bT(\by) = [T_1(\by),T_2(\by), \ldots, T_k(\by)]\in\real^k$ (i.e., a vector-valued mapping from $\mathcalY$ to $\real^k$), where $T_i : \mathcalY \to \real$, $i = 1, \ldots, k$, $S(\by) : \mathcalY \to \real$, and $\Gamma : \real^k \to \real$ are real-valued functions.
\item The sample space $\mathcalY$ does not depend on $\bphi$. 
\item The function $\Gamma(\cdot)$, known as the \textit{log partition function} or \textit{cumulant function}, is defined by the integral $\int_{\mathcalY} \exp\{ \bphi^\top \bT(\by)  + S(\by) \} d\by$.
\end{enumerate}
\end{definition}

\begin{remark}
The fact that there is an exponential in the formula \eqref{equation:exp_fam_dist} is in itself not the most important structural property of an exponential family (since any density function can be written as $f(\by) = \exp(\ln f(\by))$ on its support). 
What truly defines an exponential family is how the density can be factorized into three distinct parts: 
one that only depends on $\bphi$, i.e., $\exp(-\Gamma(\bphi))$; one that only depends on $\by$, i.e., $\exp(S(\by))$; and one that depends on both $\bphi$ and $\by$ but in a very special way: as a linear combination of the coordinates of $\bphi$ with coefficients that are functions of $\by$.
\end{remark}

\begin{remark}
The exponential family of distributions should not be confused with the exponential distribution.
Unfortunately, their names are quite similar, which can lead to confusion. To avoid ambiguity, we will always refer to the exponential \textbf{family} when discussing this broader class, distinguishing it clearly from the exponential \textbf{distribution}.
\end{remark}

We will show that all the distributions that we have so far seen, except for the uniform distribution, are members of some exponential family.
To demonstrate this, we will rewrite each distribution’s density or probability mass function into the standard form given in Equation~\eqref{equation:exp_fam_dist}. 
It will often happen that the \textit{usual parameter} $\btheta$ (for example, the mean $\mu$ and variance $\sigma^2$ in a Gaussian distribution) employed does not coincide with the natural parameter. 
However, there often exists a bijective and twice-differentiable transformation $\eta : \Theta \to \real^k$ such that $\bphi = \eta(\btheta)$. 
Consequently, the cumulant function becomes $\Gamma(\bphi) = \Gamma(\eta(\btheta)) = d(\btheta)$, where $d = \Gamma \circ \eta$

Using this transformation, the exponential family representation can be rewritten in terms of the original parameter (usual parameter) $\btheta$ as follows:
\begin{equation}\label{equation:exp_fam_dist_V2}
\underbrace{\exp \left\{ \sum_{i=1}^k \phi_i T_i(\by) - \Gamma(\bphi) + S(\by) \right\}}_{\text{natural parameterization}}
= 
\underbrace{\exp \left\{ \sum_{i=1}^k \eta_i(\btheta) T_i(\by) - d(\btheta) + S(\by) \right\}}_{\text{usual parameterization}}.
\end{equation}
\end{subequations}

Either formulation can be used, depending on which is most convenient in a specific context: for the purpose of doing theory and proving general results, the \textit{natural representation} (also called \textit{natural parametrization} or \textit{canonical representation}) given by $\exp \left\{ \sum_{i=1}^k \phi_i T_i(\by) - \Gamma(\bphi) + S(\by) \right\}$ is more convenient since  it simplifies the mathematical treatment of the model.

However, the usual parameterization often provides a more intuitive interpretation of the parameters. 
The parameters in the usual parameterization often have a direct relationship with the moments of the data, such as the mean and variance. This can make it easier to understand how changes in the parameters affect the shape of the distribution and the expected behavior of the data.
For example, in a Gaussian distribution, the mean $\mu$ and variance $\sigma^2$ have clear interpretations related to the central tendency and spread of the data. This can be beneficial for practitioners who need to communicate results to non-technical stakeholders.

The canonical parameter $\bphi$ and the original parameters of the exponential family distribution are in a one-to-one mapping relationship. The canonical parameter $\bphi$ can be a scalar parameter or a vector parameter containing two parameters. For single-parameter exponential family distributions, the original parameter is usually the mean $\mu$, and at this time $\phi$ is a function of $\mu$. For two-parameter exponential family distributions, the original parameters are usually the mean $\mu$ and the variance $\sigma^2$, and at this time $\bphi$ is a vector parameter containing these two parameters, and $\bphi$ is a function of $\mu$ and $\sigma^2$.

\begin{example}[Binomial exponential family]
Let $ \rx \sim \binoimial(n, p) $ be a Binomial random variable (Definition~\ref{definition:binomial_distri}), where $ \rx \in \{0, 1, 2, \ldots, n\} $.
Its probability mass function can be written as:
$$
\binom{n}{x} p^x (1-p)^{n-x} = \exp \left\{ \ln \left( \frac{p}{1-p} \right) x + n \ln (1-p) + \ln \binom{n}{x} \right\}.
$$
Define:
$$
\phi \triangleq \ln \left( \frac{p}{1-p} \right), \quad T(x) \triangleq x, \quad S(x) \triangleq \ln \binom{n}{x}, \quad \Gamma(\phi) \triangleq n \ln (1 + e^\phi) = -n \ln (1-p).
$$
Thus, if $ n $ is held fixed and only $ p $ is allowed to vary, the support of $ f $ does not depend on $ \phi $.
Therefore, the Binomial distribution with fixed $n$ belongs to the one-parameter exponential family.
Here the usual parameter $ p $ is a twice differentiable bijection of the natural parameter $ \phi $:
$$
p = \frac{e^\phi}{1 + e^\phi} \qquad \text{and} \qquad \phi = \eta(p) = \ln \left( \frac{p}{1-p} \right).
$$
Note that  $ p \in (0, 1) $ but $ \phi \in \real $. The sample space $\{0, 1, 2, \ldots, n\} $ does not depend on $\phi$.
\end{example}

\begin{example}[Gaussian exponential family]
Let $ \rx \sim \normal(\mu, \sigma^2) $. Its probability density function is:
$$
\frac{1}{\sigma \sqrt{2\pi}} \exp \left\{ -\frac{1}{2} \left( \frac{x-\mu}{\sigma} \right)^2 \right\}
= \exp \left\{  \frac{\mu}{\sigma^2} x -\frac{1}{2\sigma^2} x^2 - \frac{1}{2} \ln (2\pi \sigma^2) - \frac{\mu^2}{2\sigma^2} \right\}.
$$
Define:
$$
\phi_1 \triangleq \frac{\mu}{\sigma^2}, 
\ \ 
\phi_2 \triangleq -\frac{1}{2\sigma^2}, 
\ \
 T_1(x) \triangleq x, 
\ \  
T_2(x) \triangleq x^2, 
\ \ 
S(x) \triangleq 0, 
\ \ 
\Gamma(\phi_1, \phi_2) \triangleq  \frac{\ln \big( -\frac{\pi}{\phi_2} \big)}{2}-\frac{\phi_1^2}{4\phi_2}.
$$
Furthermore, the support of the distribution is always $ \real $, regardless of the values of $\mu$ and $\sigma^2$.
Hence, the normal distribution $ \normal(\mu, \sigma^2) $ is a two-parameter exponential family.
\end{example}

\begin{example}[Counterexample: uniform distribution]
Let $ \rx \sim \uniformdist(\theta_1, \theta_2) $. 
The probability density function $ f(x; \theta_1, \theta_2) $ is positive if and only if $ x \in [\theta_1, \theta_2] $. 
As a result, the support of the distribution depends on the parameters $\theta_1$ and $\theta_2$, which violates one of the key conditions required for membership in the exponential family.
Therefore, the uniform distribution does not belong to the exponential family of distributions. However, it is worth noting that if $\theta_1$ and $\theta_2$ are fixed constants (rather than variables), then the corresponding uniform distribution can technically be expressed in the exponential family form. But this would represent a degenerate case consisting of only a single distribution.
\end{example}

\subsection{Other Exponential Family Forms}\label{section:other_expfam}

We discussed the exponential family of distributions. All probability density (or mass) functions in this family can be written in the following general form:
\begin{subequations}
\begin{equation}\label{equation:ef_def2}
\textbf{(EF)}:\qquad p(\by\mid \bphi) = \exp \left\{ \bphi^\top \bT(\by) - \Gamma(\bphi) + S(\by) \right\}
\end{equation}
In this expression, $\bphi$ is called the natural parameter or canonical parameter, which represents all unknown parameters in the model. Typically, exponential family distributions involve two types of parameters: one related to the \textit{location} (such as the mean), and another related to the \textit{scale} (such as the variance).

In this chapter, we focus on generalized linear models (GLMs) that use a specific subset of the exponential family known as the \textit{natural exponential family (NEF)}. In the NEF, the function $\bT(\by)$ is simply equal to $\by$, meaning the sufficient statistic is just the data itself. This gives us the following simplified form:
\begin{equation}\label{equation:natu_exp_fam}
\textbf{(NEF)}:\qquad p(\by\mid \bphi) = \exp \left\{ \bphi^\top \by - \Gamma(\bphi) + S(\by) \right\} 
\end{equation}
This version is often referred to as the \textit{natural form} or \textit{canonical form} of the exponential family. While most commonly used exponential family distributions can be expressed in this form, there are exceptions---such as the log-normal distribution---that belong to the broader exponential family but cannot be written in the natural form.

\begin{remark}
There is a potential source of confusion here: although the parameter $\bphi$ is already called the canonical (natural) parameter, the entire expression is only said to be in canonical (natural) form if $\bT(\by) = \by$. So both conditions must be satisfied for something to be in canonical form.
\end{remark}

In the exponential family, some distributions have only one parameter, while others have two parameters. The natural parameter $\bphi$ contains all the usual parameters of the distribution. When the distribution has only one parameter, $\phi$ is a scalar parameter. When the distribution has two parameters, $\bphi$ is a two-dimensional vector parameter. The two parameters of the exponential family distribution are related to the mean and variance of the distribution, representing the location and scale, respectively.

The natural form of the exponential family (Equation~\eqref{equation:natu_exp_fam}) contains all relevant parameters together, which can make it cumbersome to work with in practice. To simplify things, we introduce an additional structure by decomposing the parameters and including a separate \textit{dispersion parameter} $\rho$.

\index{Exponential dispersion family}
\index{Natural parameter}
\index{Usual parameter}
\begin{definition}[The exponential dispersion family (EDF) of distributions]
A regular probability distribution is said to be a member of a \textit{exponential dispersion family}, 
if its density (or mass function) can be written as:
\begin{equation}\label{equation:exp_fam_edf}
\textbf{(EDF)}:\qquad p(y\mid \phi, \rho) = \exp \left\{ \frac{ y\phi- b(\phi)}{a(\rho)} + c(y, \rho) \right\},\quad y\in\mathcalY,
\end{equation}
where:
\begin{enumerate}
\item $\phi\in\real$ is called the \textit{natural parameter} or \textit{canonical parameter}.
\item $a:\real\rightarrow \real$, $b:\real\rightarrow \real$, and $c:\real\times \real\rightarrow \real$ are functions. The function $b$ is known as the \textit{log normalizer}, also called the \textit{cumulant function}.
\item $a(\rho)>0$ is the \textit{dispersion function}, which is known, and $\rho$ is the \textit{dispersion parameter}. In many settings, $\rho$ is not the main focus of analysis and may be treated as a ``nuisance" parameter.
\end{enumerate}
\end{definition}

\end{subequations}

\begin{remark}\label{remark:dips_fac_inedf}
In most applications of generalized linear models, we typically assume that $a(\rho) = \rho$. 
However, in some cases, we may use $a(\rho) = \rho / w_i$ for all $i\in\{1,2,\ldots,n\}$, where $w_i$ represents the sample or group weight; see Section~\ref{section:group_glm}. This allows each observation to have a different weight, and the value of $w_i$ is assumed to be known.

The use of weights is not always necessary---only when the specific application requires assigning different importance or precision to individual observations. In such cases, the weights are known in advance; see Example~\ref{example:grouped_logis}. Therefore, many introductory materials on GLMs omit the weights and simply assume $a(\rho) = \rho$.
\end{remark}
Note that Equation~\eqref{equation:exp_fam_edf} forms by decomposing the parameter $\bphi$ from the NEF. 
This decomposition separates the components related to the mean and the variance. As a result:
\begin{itemize}
\item The natural parameter $\phi$ becomes associated only with the mean $\mu$.
\item While the dispersion parameter $\rho$ is associated with the variance.
\end{itemize} 
After this  decomposition, there is a one-to-one functional relationship between the natural parameter $\phi$ and the mean parameter $\mu$, allowing us to convert between them using a link function $g$:
$$
\phi=g(\mu)
\qquad \text{and}\qquad 
\mu = g^{-1}(\phi) \triangleq h(\phi).
$$

For any valid probability distribution, it must integrate (or sum, in the discrete case) to 1 over its entire sample space. That is,
\begin{equation}\label{edf_sumone}
	\int p(y\mid \phi, \rho) \, dy = 1 
\end{equation}
Applying this condition to the EDF form (Equation~\eqref{equation:exp_fam_edf}), we get:
\begin{align}
\int \exp \left[ \frac{y\phi - b(\phi)}{a(\rho)} + c(y, \rho) \right] \, dy &= 1 \nonumber \\
\implies \exp \left\{-\frac{b(\phi)}{a(\rho)}\right\} \int \exp \left( \frac{y\phi}{a(\rho)} + c(y, \rho) \right) \, dy &= 1 \nonumber\\
\implies \frac{b(\phi)}{a(\rho)} = \ln \int \exp \left( \frac{y\phi}{a(\rho)} + c(y, \rho) \right) \, dy.& \label{equation:edf_sumone_cnt}
\end{align}
Equation~\eqref{equation:edf_sumone_cnt} determines $ b $ in terms of dispersion function $ a(\rho) $ and the function $ c(y, \rho) $. 
Thus, $b(\phi)$ is not arbitrary---it is determined by the normalization requirement. It is commonly referred to as the log normalizer, although it's important to remember that its value also depends on the dispersion parameter $\rho$.

\paragrapharrow{Mean of EDF.} If we differentiate \eqref{equation:edf_sumone_cnt} with respect to $ \phi $, we obtain
\begin{equation}
\frac{b'(\phi)}{a(\rho)} = \frac{\int \frac{y}{a(\rho)} \exp \left( \frac{y\phi}{a(\rho)} + c(y, \rho) \right) \, dy}{\int \exp \left( \frac{y\phi}{a(\rho)} + c(y, \rho) \right) \, dy}.
\end{equation}
Now, using Equation~\eqref{equation:edf_sumone_cnt}, we can substitute the denominator on the right-hand side to simplify this expression:
\begin{equation}\label{equation:edf_sumone_mean2}
\begin{aligned}
\frac{b'(\phi)}{a(\rho)} &= \int \frac{y}{a(\rho)} \frac{\exp \left( \frac{y\phi}{a(\rho)} + c(y, \rho) \right)}{\exp \left( \frac{b(\phi)}{a(\rho)} \right)} \, dy 
= \int \frac{y}{a(\rho)} \exp \left( \frac{y\phi - b(\phi)}{a(\rho)} + c(y, \rho) \right) \, dy  \\
&= \int \frac{y}{a(\rho)} p(y\mid \phi, \rho) \, dy 
= \frac{1}{a(\rho)} \Exp[\ry\mid \phi, \rho].
\end{aligned}
\end{equation}
Therefore, we conclude that:
\begin{equation}
b'(\phi) = \Exp[\ry\mid \phi, \rho].
\end{equation}
It turns out that $ b' $ is almost always invertible for finite parameter values, because its derivative $ b'' > 0 $ except when the variance of the distribution is zero (see discussion in \eqref{equation:edf_sumone_var1}). Thus, we can write:
\begin{equation}\label{equation:edf_sumone_mean3}
\mu \triangleq \Exp[\ry\mid \phi, \rho] = b'(\phi)  
\qquad\iff\qquad 
\phi = (b')^{-1}(\mu) \triangleq \psi(\mu).
\end{equation}
This shows that the EDF distribution can therefore be parameterized in terms of the natural parameter $ \phi $ or in terms of the mean $ \mu $, as discussed earlier.

In practice, we define $\psi \triangleq(b')^{-1}$ such that $\psi(\mu) = \phi$, where $\mu$ is the mean, $\phi$ is the natural parameter, and $\psi$ relates the mean with the natural parameter; see the right part of  Figure~\ref{fig:glm_rela}.
The equivalences in Equation~\eqref{equation:edf_sumone_mean3} are especially important from the perspective of functional models---and generalized linear models in particular---because they establish a one-to-one correspondence between the distribution's parameterization and its expected value.

\paragrapharrow{Variance of EDF.}
From \eqref{equation:edf_sumone_mean2}, we have that
$$
b'(\phi) = \exp \left( -\frac{b(\phi)}{a(\rho)} \right) \int y \exp \left( \frac{y\phi}{a(\rho)} + c(y, \rho) \right) dy.
$$
Differentiating again (using the product rule), we obtain:
\begin{equation}\label{equation:edf_sumone_var1}
\begin{aligned}
b''(\phi) &= -\frac{b'(\phi)}{a(\rho)} b'(\phi) + \exp \left( -\frac{b(\phi)}{a(\rho)} \right) \int \frac{y^2}{a(\rho)} \exp \left( \frac{y\phi}{a(\rho)} + c(y, \rho) \right) dy \\
&= -\frac{\mu^2}{a(\rho)} + \frac{1}{a(\rho)} \Exp[\ry^2\mid \phi, \rho]
= \frac{1}{a(\rho)} \Var[\ry\mid \phi, \rho] .
\end{aligned}
\end{equation}
Note that \eqref{equation:edf_sumone_var1} shows that $ b'' \geq 0 $, with equality only if the variance is zero or the dispersion is infinite.
We can now reparameterize in terms of the mean $ \mu $:
\begin{equation}\label{equation:edf_sumone_var2}
\begin{aligned}
\Var[\ry\mid \phi, \rho] &= a(\rho) b''(\phi) 
= a(\rho) b''\big((b')^{-1}(\mu)\big) 
= a(\rho) \mathcalV(\mu),
\end{aligned}
\end{equation}
where 
\begin{equation}
\mathcalV(\cdot) \triangleq b''\big((b')^{-1}(\cdot)\big) 
\end{equation}
is called the \textit{variance function}. 
Equations~\eqref{equation:edf_sumone_mean3} and \eqref{equation:edf_sumone_var2} make it clear why $ a(\rho) $ is called the ``dispersion." Its value does not affect the mean $ \mu = \Exp[\ry\mid \phi, \rho] $, but it scales the variance $ \Var[\ry\mid \phi, \rho] $.

The variance function is defined as the second derivative of the cumulant function $b(\phi)$.
There are two cases for the variance function $\mathcalV(\mu)$ to consider:
\begin{enumerate}[(i)]
\item The variance function is a constant value, $\mathcalV(\mu) = b''(\phi) = \text{constant}$. In this case, the variance of the distribution does not depend on the mean.
\item The variance function is a function of the mean $\mu$, $\mathcalV(\mu) = b''(\phi)$.
\end{enumerate}

The derivation of \eqref{equation:edf_sumone_var2} also shows that since $b(\phi)$ is decomposed from $\Gamma(\bphi)$ in the NEF form of \eqref{equation:natu_exp_fam} by removing $a(\rho)$, the second derivative of $b(\phi)$ no longer represents the variance of the distribution. It needs to be multiplied by $a(\rho)$ again to get the variance of the distribution.

\begin{example}[Exponential, Poisson, Bernoulli, Gaussian in EDF forms]\label{example:edf_examps}
The \textit{exponential distribution} for $x\in [0, +\infty)$ is $p(x\mid \lambda) = \lambda \exp\{ -\lambda x\} = \exp\{ - \lambda x + \ln \lambda\}$,
which is an EDF with
$$
\phi \triangleq - \lambda, 
\qquad a(\rho)\triangleq\rho = 1,
\qquad b(\phi) \triangleq -\ln \lambda=-\ln(-\phi), 
\qquad c(x, \rho) \triangleq 0.
$$

The \textit{Bernoulli distribution} for $ x \in \{0, 1\} $ is
$
p(x\mid \pi)= \pi^x (1 - \pi)^{1-x} = \exp \{ x \ln \frac{\pi}{1-\pi} + \ln(1 - \pi)\},
$
which is an EDF with
$$
\phi \triangleq \ln \frac{\pi}{1-\pi}, 
\qquad a(\rho)\triangleq\rho = 1, 
\qquad b(\phi) \triangleq -\ln(1 - \pi) = \ln(1 + e^\phi),
\qquad c(x, \rho) \triangleq 0.
$$

The \textit{binomial distribution} for $ x \in \{0, 1, 2, \ldots, n\} $ with fixed $n$ is  $\binom{n}{x} \pi^x (1-\pi)^{n-x} = 
\exp \big\{ \ln \big( \frac{\pi}{1-\pi} \big) x + n \ln (1-\pi) + \ln \binom{n}{x} \big\}$,
which is an EDF with
$$
\phi \triangleq \ln \big( \frac{\pi}{1-\pi} \big), 
\quad a(\rho)\triangleq\rho = 1, 
\quad b(\phi) \triangleq -n\ln(1 - \pi)=-n\ln\big( \frac{e^{-\phi}}{1+e^{-\phi}} \big),
\quad c(x, \rho) \triangleq \ln \binom{n}{x}.
$$

%The \textit{Gamma distribution} for $x\geq 0$ is $\frac{\lambda^r}{\Gamma(r)} x^{r-1} \exp(-\lambda x)$

The \textit{Poisson distribution} for $ x \in \mathbb{N} $ is
$
p(x\mid \lambda) = \frac{\lambda^x e^{-\lambda}}{x!} = \exp[x \ln \lambda - \lambda - \ln x!],
$
which is an EDF with
$$
\phi \triangleq \ln \lambda,
\qquad a(\rho)\triangleq\rho = 1,
\qquad b(\phi) \triangleq \lambda = e^\phi,
\qquad c(x, \rho) \triangleq -\ln x!.
$$

The \textit{Gaussian distribution} for $ x \in \real $ is
\begin{equation}
\small
\begin{aligned}
p(x\mid \mu, \sigma^2)= \frac{1}{\sqrt{2\pi\sigma^2}} \exp \left\{-\frac{1}{2\sigma^2}(x - \mu)^2\right\}
= \exp \left\{\frac{x\mu - \frac{1}{2}\mu^2}{\sigma^2} - \frac{x^2}{2\sigma^2} - \frac{1}{2}\ln(2\pi\sigma^2)\right\},
\end{aligned}
\end{equation}
which is an EDF with
$$
\phi \triangleq \mu, 
\qquad a(\rho)\triangleq\rho = \sigma^2,
\qquad b(\phi) \triangleq \frac{1}{2}\mu^2 = \frac{1}{2}\phi^2,
\qquad c(x, \rho) \triangleq -\frac{x^2}{2\rho} - \frac{1}{2}\ln(2\pi\rho).
$$
Further results are discussed in Problem~\ref{prob:edf_oths}.

In GLMs, if the parameter $\phi$ is the only unknown parameter in the model, it is referred to as  a \textit{single-parameter model} (which  aligns with the conventional definition of a one-parameter natural EDF in its canonical form). A single-parameter model means that only $\phi$ is an unknown parameter, whereas $\rho$ is considered known. 
Conversely, if both $\phi$ and $\rho$ are unknown, it becomes a \textit{two-parameter model}. Some distributions in the EDF do not incorporate a dispersion parameter. Examples include the Bernoulli distribution, the Poisson distribution, binomial distribution, and the exponential distribution.

Once again, these examples show that the natural parameter $\phi$ is related to the mean of the distribution, and it is a function of the mean. The dispersion parameter $\rho$ is related to the variance of the distribution, and it affects the magnitude of the variance.
\end{example}

\section{Generalized Linear Models (GLMs)}

Recall that we are given:
\begin{itemize}
\item Predictors $\bx \in \mathcal{X} \subseteq \real^p$.
\item A response $y \in \mathcal{Y} \subseteq \real$. This may be numerical; continuous or discrete, or it may be binary. 
\end{itemize}
We suppose that there is some functional relationship between the predictors and the response, i.e. that $p(y\mid \bx, f) = p(y\mid f(\bx))$, with $\Exp[\ry\mid \bx, f] = f(\bx)$, for some $f \in \mathcalF$, a suitable set of possible functions.

One class of models of functional relationships, defined via a set of possible functions $\mathcalF$ and a set of possible probability distributions whose means will be controlled by those functions, are the \textit{generalized linear models (GLMs)}, which we now define.

\subsection{Key Components of GLMs}

When studying statistical analysis, data mining, machine learning, and related fields, the first model we typically encounter is linear regression. In addition to linear regression, there are other models such as logistic regression, Poisson regression, and binomial regression, which are also fundamentally linear in nature. Although these models were developed by different researchers at different times, they all belong to the same model family. In 1972, \citet{nelder1972generalized} introduced a unifying framework called generalized linear models (GLMs). This framework allows for the incorporation of various common regression models, enabling their parameters to be estimated using a unified method.
Generalized linear models rely on the natural form of the exponential dispersion family to construct models for a random variable $\ry$. Different regression models correspond to different distributions of $\ry$ within this family.

Within the GLM framework, we assume that the response variable $\ry$  follows a distribution from the exponential dispersion family. The goal is to predict the value of $\ry$  based on the input variables $\bx$. As a linear model, GLM achieves this by using a linear combination of the input features $\bx$ to make predictions.
A GLM is defined by the following three components:
\begin{subequations}
\begin{enumerate}[(i)]
\item A \textit{linear predictor} or a \textit{systematic component} for each data point $\bx$:
\begin{equation}
\eta = \bbeta^\top \bx + \beta_0.
\end{equation}
That is, $\eta$ is a linear function of $\bx$. To simplify notation, we often append a constant value of 1 to the input vector $\bx$, allowing the intercept or bias term $\beta_0$ to be included in the parameter vector $\bbeta$. With this adjustment, the linear predictor can be expressed more compactly as:
\begin{equation}
\eta = \bbeta^\top \bx.
\end{equation}

\item The distributional assumption or the \textit{random component}: Our knowledge of the response variable $\ry$ given the input $\bx$ and parameters $\bbeta$ is modeled using an EDF, where the parameters  depend on $\bx$ and $\bbeta$:
\begin{equation}\label{equation:glm_comp3}
p(y\mid \bx, \bbeta) = p\big(y\mid \phi(\bx, \bbeta), \rho(\bx, \bbeta)\big) = \exp\left(\frac{y\phi - b(\phi)}{a(\rho)} + c(y, \rho)\right).
\end{equation}
Furthermore, we assume that the values of $\ry$ corresponding to different $\bx$ and $\bbeta$ are independent of each other and other values of $\bx$ (but not of $\bbeta$). That is:
\begin{equation}
	p(\{y_i\}\mid \{\bx_i\}, \bbeta) = \prod_{i=1}^n p(y_i\mid \bx_i, \bbeta).
\end{equation}
where $\{y_i, i=1,2, \ldots,n\}$ represents the response data corresponding to the input vectors $\{\bx_i, i=1,2, \ldots,n\}$.
In simpler terms, the random component involves selecting a suitable probability distribution for the response variable, e.g., discrete or continuous.

\item Recall the initial goal of a linear model: we aim to use the linear predictor $\eta$, derived from the input variables $\bx$, to predict the output variable $\ry$. 
Here, $\ry$ is a random variable following from a distribution from the EDF. 
For a random variable, its value can be any value in its domain space, with each value having a different probability (of course, for a uniformly distributed variable, each value has the same probability). 
However, for prediction purposes, we seek a specific value---typically the expected value of $\ry$, which provides the best estimate in many cases: $\mu\triangleq \Exp[\ry\mid \bx, \bbeta]$.
Now, our task is to derive $\mu$ from $\eta$, and then use $\mu$ as the output value of the model, i.e., the predicted value of $\ry$. 
To do this, we define an \textit{injective response function} (or simply a \textit{response function} or an \textit{activation function}) $h$, such that
\begin{equation}
	\mu = \Exp[\ry\mid \bx, \bbeta] = h(\eta) = h(\bbeta^\top \bx).
\end{equation}
Equivalently, we can write
\begin{equation}
	g(\mu) = \bbeta^\top \bx = \eta.
\end{equation}
where $g \triangleq h^{-1}$ is called the \textit{link function}. Essentially, the link function maps the real-valued linear predictor $\eta$ to the valid range of the mean $\mu$  for the chosen distribution (e.g., a positive space).
Commonly used link and response functions are summarized in Table~\ref{tab:link_functions}.
Because we are choosing a distribution from the EDF, specifying the mean $\mu=h(\bbeta^\top\bx)$ via the response function alone is not sufficient. As shown in Equation~\eqref{equation:edf_sumone_mean3}, we also need a transformation function $\psi$ to map the mean parameter to the natural parameter of the EDF distribution.
\end{enumerate}
\end{subequations}
A generalized linear model extends the classical linear regression model framework by allowing the conditional distribution of the response variable $\ry$ to belong to the broader class of  EDF distributions. Figure~\ref{fig:glm_rela} illustrates the relationships among the variables with the GLM framework.
\begin{figure}[h]
	\centering
	\includegraphics[width=\textwidth]{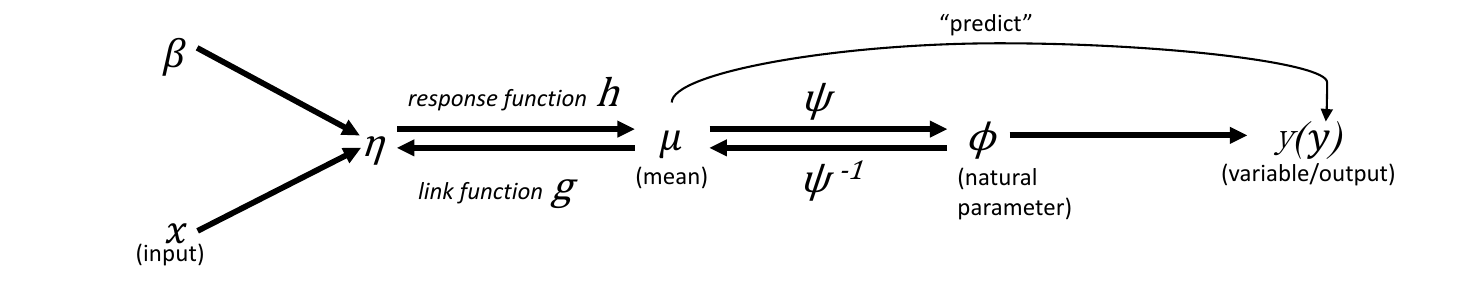}
	\caption{Relationships between variables in a generalized linear model. 
		The input variable $\bx$ and coefficient $\bbeta$ form a linear relationship, $\eta = \bbeta^\top \bx$. $\eta$ is called the linear predictor, and $\bbeta$ is an unknown parameter.}
	\label{fig:glm_rela}
\end{figure}

\paragrapharrow{Connecting the three components.}
In the framework of a generalized linear model, the response variable $\ry$ is treated as a random variable with its probability distribution being an EDF distribution $p(y\mid \phi, \rho) = \exp \left\{ \frac{\phi y - b(\phi)}{a(\rho)} + c(y, \rho) \right\} $, where $\phi$ is the natural parameter of the distribution. 
There exists a one-to-one correspondence between the natural parameter $\phi$ and the mean parameter $\mu$, which we denote by the function $\psi$: $\phi\triangleq \psi(\mu)$.
The linear predictor $\eta=\bbeta^\top\bx$ and the mean $\mu$ of the EDF distribution have a functional relationship, $\mu = g^{-1}(\eta) \triangleq h(\eta)$. 
Therefore, the natural parameter $\phi$ of the EDF distribution can always be transformed into a function of $\eta$, and the probability distribution function of the response variable $\ry$ can be converted into a function related to $\eta$:
\begin{tcolorbox}[colback=white,colframe=black]
\begin{minipage}{1\textwidth}
\begin{equation}\label{equation:gen_glm}
\small
\begin{aligned}
\textbf{(GLM)}:\quad 
p(y\mid \phi, \rho) &= \exp \left\{ \frac{\phi y - b(\phi)}{a(\rho)} + c(y, \rho) \right\} 
= \exp \left\{ \frac{\psi(\mu) y - b(\psi(\mu))}{a(\rho)} + c(y, \rho) \right\} \\
&= \exp \left\{ \frac{\psi(g^{-1}(\eta)) y - b\big\{\psi[g^{-1}(\eta)]\big\}}{a(\rho)} + c(y, \rho) \right\} \\
&= \exp \left\{ \frac{\psi\big(g^{-1}(\bbeta^\top \bx)\big) y - b\big\{\psi\big[g^{-1}(\bbeta^\top \bx)\big]\}}{a(\rho)} + c(y, \rho) \right\} 
= p(y\mid \bx,  \bbeta).
\end{aligned}
\end{equation}
\end{minipage}
\end{tcolorbox}
\noindent
So far, we have connected the input variable $\bx$ and the response variable $\ry$ through their probability distribution functions, resulting in the conditional probability distribution $p(y\mid \bx, \bbeta)$. Equation~\eqref{equation:gen_glm} represents the general form of a generalized linear model.

\paragrapharrow{Link functions.}
As mentioned earlier, in the GLM framework, the function $g$ is called the link function. It connects the linear predictor $\eta$ to the mean $\mu$. The inverse function of the link function $h\triangleq g^{-1}$ is called the response function (a.k.a., the activation function).
There are various options for the link function, depending on the distribution of the response variable:
\begin{itemize}
\item In the Gauss-Markov linear model (i.e., the standard Gaussian linear regression model), the link function is the identity function $\eta = g(\mu) = \mu$.  The variance function is constant, i.e., $\Var[\mu] = 1$, and the dispersion parameter is the variance, i.e., $\rho = \sigma^2$ (see Example~\ref{example:edf_examps}), allowing the use of ordinary least squares in parameter estimation in procedures such as linear regression, analysis of variance (ANOVA) models, or analysis of covariance (ANCOVA) models.

\item In the Poisson distribution, the mean $\mu$ must be positive, so $\eta = \mu$ is no longer applicable because $\eta = \bbeta^\top \bx$ can take any real value. For the Poisson distribution, the link function can choose the log function $\eta = \ln \mu$, at the same time $\mu = e^\eta$ ensures that $\mu$ is positive. 
The variance of the function has the form $\Var[\mu] = \mu$, and the dispersion parameter is 1. Poisson models with a log link function are often referred to as log-linear models, commonly used when there are contingency (data frequency) tables with at least two entries.

\item In a binomial distribution, the response variable takes binary values like 0 and 1 or represents the relative frequency, i.e., $y_i = e_i/n_i$, where $e_i$ is the number of successes and $n_i$ is the number of trials. The mean is a probability ($\pi\equiv \mu$) and therefore must be between 0 and 1. The linear predictor is not bounded. Therefore, the link function must map the real line within the interval [0, 1]. A natural link function for binomial data is the logit link:
$$
\eta = \ln\left(\frac{\pi}{1-\pi}\right) 
\quad\implies\quad  
\pi = \frac{e^\eta}{1 + e^\eta},
\quad \pi\in(0,1).
$$
Another useful alternative for these types of data is the probit link function; see Table~\ref{tab:logit-probit-comparison}:
$$
\eta = \Phi^{-1}(\pi) 
\quad\implies\quad   
\pi = \Phi(\eta),
$$
where $\Phi$ is the cumulative distribution function of a standard normal distribution. The variance of the function has the form $\Var[\pi] = (\pi/(1 - \pi))$ and the dispersion parameter $\rho$---the same as the Poisson distribution---is known and is equal to 1 ($\rho = 1$). 
\end{itemize}
The response function, on the other hand, can be linear or nonlinear. For example, the response function in a standard Gaussian linear regression model is $\mu = g^{-1}(\eta) = \eta$, and the response function in a logistic regression model is $\mu = g^{-1}(\eta) = \text{sigmoid}(\eta)$; see Section~\ref{section:logis_reg}.
Several common link and response functions are summarized in Table~\ref{tab:link_functions}, and typical types of GLMs are outlined in Tables~\ref{tab:common_glmmodels1} and \ref{tab:common_glmmodels2}.

\begin{table}[h]
\centering
\caption{Common link functions of $g$, and their inverse functions---the response function $h$. {Note that $\Phi$ is the cumulative distribution function of a standard normal distribution; $\mu$ is the expected values of the response; $\eta$ is the linear predictor; and $\phi$ is the dispersion parameter}.}
\label{tab:link_functions}
\resizebox{1\textwidth}{!}{%}
\begin{tabular}{|c|c|c|c|}
\hline
\textbf{Name} & \textbf{Link function ($g$)} & \textbf{Response function ($h$)} & \textbf{Domain of $\mu$} \\
\hline\hline
Identity & $\eta = \mu$ & $\mu = \eta$ & $\mu \in \real$ \\
\hline
Log & $\eta = \ln(\mu)$ & $\mu = e^\eta$ & $\mu > 0$ \\
\hline
Logit & $\eta = \ln\left\{\frac{\mu}{1 - \mu}\right\}$ & $\mu = \frac{e^\eta}{1 + e^\eta}$ & $\mu \in (0, 1)$ \\
\hline
Probit & $\eta = \Phi^{-1}(\mu)$ & $\mu = \Phi(\eta)$ & $\mu \in (0, 1)$ \\
\hline
Negative binomial($\alpha$) & $\eta = \ln\left\{\frac{\mu}{\mu + 1/\alpha}\right\}$ & $\mu = \frac{e^\eta}{\alpha(1 - e^\eta)}$ & $\mu > 0$ \\
\hline
Log-complement & $\eta = \ln(1 - \mu)$ & $\mu = 1 - e^\eta$ & $\mu < 1$ \\
\hline
Log-log & $\eta = -\ln\{-\ln(\mu)\}$ & $\mu = \exp\{-\exp(-\eta)\}$ & $\mu \in (0, 1)$ \\
\hline
Complementary log-log & $\eta = \ln\{-\ln(1 - \mu)\}$ & $\mu = 1 - \exp\{-\exp(\eta)\}$ & $\mu \in (0, 1)$ \\
\hline
Reciprocal & $\eta = 1/\mu$ & $\mu = 1/\eta$ & $\mu \in \real, \mu\neq 0$ \\
\hline
Power($\alpha = -2$) & $\eta = 1/\mu^2$ & $\mu = 1/\sqrt{\eta}$ & $\mu > 0$ \\
\hline
Power($\alpha$) & $\begin{cases} 
\eta = \mu^\alpha & \alpha \neq 0 \\
\eta = \ln(\mu) & \alpha = 0 
\end{cases}$ & $\mu = \begin{cases} 
\eta^{1/\alpha} & \alpha \neq 0 \\
\exp(\eta) & \alpha = 0 
\end{cases}$ & $\mu \in \real$ \\
\hline
Odds power($\alpha$) & $\begin{cases} 
\eta = \frac{[\mu/(1-\mu)]^\alpha - 1}{\alpha} & \alpha \neq 0 \\
\eta = \ln\left(\frac{\mu}{1-\mu}\right) & \alpha = 0 
\end{cases}$ & $\mu = \begin{cases} 
\frac{(1+\alpha\eta)^{1/\alpha}}{1+(1+\alpha\eta)^{1/\alpha}} & \alpha \neq 0 \\
\frac{e^\eta}{1+e^\eta} & \alpha = 0 
\end{cases}$ & $\mu \in (0, 1)$ \\
\hline
\end{tabular}
}
\end{table}

\begin{table}[h]
\centering
\renewcommand{\arraystretch}{1.25}
\caption{Common types of GLMs (Part 1/2).}
\label{tab:common_glmmodels1}
\resizebox{1\textwidth}{!}{%}
\begin{tabular}{|c|c|c|c|}
\hline
& Gaussian $\normal(\mu, \sigma^2)$                                                             & Exponential $\exponential(\lambda)$  & Categorical $Cat(K, \mu)$                                                                                                                                                                                                                                                                                               \\ \hline\hline
Range of $\ry$               & real: $(-\infty, +\infty)$                                                                    & Nonnegative: $[0, +\infty)$          & $\{1, 2,\ldots, K\}$                                                                                                                                                                                                                                                                                                    \\ \hline
$p(\ry)$                     & $\frac{1}{\sqrt{2\pi\sigma^2}}\exp\left\{-\frac{(y-\mu)^2}{2\sigma^2}\right\}$                & $\lambda \exp\{ -\lambda y\}$        & $\prod_k \mu_k^{y_k}$                                                                                                                                                                                                                                                                                                   \\ \hline
EDF                          & $\exp\left\{\frac{\mu y - \mu^2}{\sigma^2} - \frac{y^2}{2\sigma^2} - \ln2\pi\sigma^2\right\}$ & $\exp\{ - \lambda y + \ln \lambda\}$ &\parbox{10em}{$\exp\Big\{\sum_{k=1}^{K-1} x_k \ln \left(\frac{\mu_k}{\mu_K}\right)$ \\$+\ln \left(1 - \sum_{k=1}^{K-1} \mu_k\right)\Big\}$} \\ \hline
$\phi = \psi(\mu)$           & $\phi = \mu$                                                                                  & $\phi = \ln (\mu)$                   & $\theta_k = \ln \left(\frac{\mu_k}{\mu_K}\right)$                                                                                                                                                                                                                                                                       \\ \hline
$\mu = \psi^{-1}(\phi)$      & $\mu = \phi$                                                                                  & $\mu = e^{\phi}$                       & $\mu_k = \frac{e^{\theta_k}}{\sum_{j=1}^K e^{\theta_j}}$                                                                                                                                                                                                                                                                \\ \hline
$b(\phi)$                    & $\frac{\phi^2}{2}$                                                                            &    $-\ln (\phi)$                & $\ln \left(\sum_{k=1}^K e^{\theta_k}\right)$                                                                                                                                                                                                                                                                            \\ \hline
Link name                    & Identity                                                                                      & Reciprocal                              & Logit                                                                                                                                                                                                                                                                                                                   \\ \hline
Link function                & $\eta = \mu$                                                                                  & $\eta = 1/\mu$                       & $\eta_k = \ln \left(\frac{\mu_k}{\mu_K}\right)$                                                                                                                                                                                                                                                                         \\ \hline
Mean function                & $\mu = \eta$                                                                                  & $\mu$                                & $\mu_k = \frac{e^{\eta_k}}{\sum_k e^{\eta_k}}$                                                                                                                                                                                                                                                                          \\ \hline
$\mathcalV(\mu) = b''(\phi)$ & $1$                                                                                           &           $-1/(\ln \mu)^2$                   & $\mu_k (1 - \mu_k)$                                                                                                                                                                                                                                                                                                     \\ \hline
$a(\rho)$                    & $\sigma^2$                                                                                    &        1                        & 1                                                                                                                                                                                                                                                                                                                       \\ \hline
\end{tabular}
}

\vspace{10pt} % Adds space between the tables

\centering
\renewcommand{\arraystretch}{1.15}
\caption{Common types of GLMs (Part 2/2).}
\label{tab:common_glmmodels2}
\resizebox{1\textwidth}{!}{%}
\begin{tabular}{|c|c|c|c|}
\hline
& Poisson $\poissondist(\mu)$        & Bernoulli $\bernoulli(\mu)$                                        & Binomial $\binomialdist(N, \mu)$     \\ \hline\hline
Range of $\ry$                & integer $0, 1, 2, \ldots$     & $\{0, 1\}$                                                & $\{0,1, \ldots, N\}$       \\ \hline
$p(\ry)$                      & $\exp\{y \ln \mu - \ln \mu\}$ & $\mu^y(1 - \mu)^{1-y}$                                    & $\binom{N}{y}\mu^y(1 - \mu)^{N-y}$                                \\ \hline
EDF                       & $\exp\{y \ln \mu - \ln \mu\}$ & $\exp\left\{y\ln\frac{\mu}{1-\mu} + \ln(1-\mu)\right\}$   & $\exp\left\{\frac{y\ln\frac{\mu}{1-\mu} + \ln(1-\mu)}{1/N}\right\}$ \\ \hline
$\phi = \psi(\mu)$      & $\phi = \ln \mu$            & $\phi = \ln\left(\frac{\mu}{1-\mu}\right) = \text{logit}(\mu)$ & $\phi = \ln\left(\frac{\mu}{1-\mu}\right)$                      \\ \hline
$\mu = \psi^{-1}(\phi)$ & $\mu = e^\phi$              & $\mu = \frac{1}{1+e^{-\phi}} = \text{sigmoid}(\phi)$         & $\mu = \frac{1}{1+e^{-\phi}}$                                   \\ \hline
$b(\phi)$               & $e^\phi$                    & $\ln(1 + e^\phi)$                                       & $\ln(1 + e^\phi)$      \\ \hline
Link name                 &  Log                             & Logit                                                     & Logit                    \\ \hline
Link function             &    $\ln(\mu)$                           & $\eta = \ln\left(\frac{\mu}{1-\mu}\right)$                & $\eta = \ln\left(\frac{\mu}{1-\mu}\right)$                        \\ \hline
Mean function             &   $\mu$                          & $\mu = \frac{1}{1+e^{-\eta}}$                             & $\mu = \frac{N}{1+e^{-\eta}}$                                     \\ \hline
$\mathcalV(\mu) = b''(\phi)$    & $\mu$                            & $\mu(1 - \mu)$                                            & $\mu(1 - \mu)$           \\ \hline
$a(\rho)$                 & 1                             & $1$                                                       & $\frac{1}{N}$            \\ \hline
\end{tabular}
}
\end{table}

\index{Natural link}
\index{Canonical link}
\subsection{The Natural/Canonical Link}\label{section:glm_natural_cacno_link}

There is a particular choice of the response (or equivalently, link) function that significantly simplifies the model formulation.
This is known as the \textit{natural link} or \textit{canonical link}.
Recall that we have the following two expressions for the mean:
\begin{subequations}\label{equation:glm_mu_all}
\begin{align}
	\mu &= \Exp[\ry\mid \phi, \rho] = b'(\phi)  \label{equation:glm_mu_v1} ; \\
	\mu &= \Exp[\ry\mid \bx, \bbeta] = h(\bbeta^\top \bx) = h(\eta) \label{equation:glm_mu_v2},
\end{align}
\end{subequations}
where \eqref{equation:glm_mu_v1} holds because $p(y\mid \phi, \rho)$ follows an EDF distribution (see Section~\ref{section:other_expfam}), and \eqref{equation:glm_mu_v2} holds by the definition of a GLM. Following \eqref{equation:glm_mu_all}, we have that
\begin{equation}
	\phi = (b')^{-1}(\mu)  \triangleq \psi(\mu) = \psi\big(h(\bbeta^\top \bx)\big).
\end{equation}
The \emph{natural link} (or \textit{canonical link}) is defined by choosing  $h \triangleq b'$, or equivalently $g \triangleq \psi$, resulting in the equality
\begin{equation}
	\phi = \bbeta^\top \bx = \eta = g(\mu) = \psi(\mu).
\end{equation}
As a result, the general GLM form in Equation~\eqref{equation:gen_glm} becomes
\begin{equation}
\textbf{(Natural GLM)}:\qquad 	p(y\mid \bx, \bbeta) = \exp \left\{ \frac{(\bbeta^\top \bx) y - b(\bbeta^\top \bx)}{a(\rho)} + c(y, \rho) \right\} .
\end{equation}
It is evident that this choice greatly simplifies the model formulation. However, other link functions can still be used when necessary, particularly since the natural link may sometimes exhibit undesirable properties.
Using the canonical link function provides several statistical advantages, the most immediate being the simplification of parameter estimation procedures.

\begin{exercise}
Following Example~\ref{example:edf_examps}, derive the corresponding forms for each distribution using the natural link function.
\end{exercise}

\begin{example}[Standard Gauss-Markov linear model]
The standard (Gauss-Markov) linear regression model assumes that the response variable $\ry$ follows a Gaussian distribution. The probability density function of the Gaussian distribution in the form of an EDF is expressed as:
\begin{equation}
	p(y \mid \mu, \sigma^2) = \exp \left\{ \frac{y \mu - \frac{1}{2} \mu^2}{\sigma^2} - \frac{y^2}{2 \sigma^2} - \frac{1}{2} \ln(2 \pi \sigma^2) \right\}
\end{equation}
Example~\ref{example:edf_examps} shows that the standard terms are:
\begin{align*}
	\phi &= \mu;  \qquad &b(\phi) &= \frac{1}{2} \phi^2; \\
	a(\rho) &=\rho= \sigma^2;   \qquad &c(y, \rho) &=- \frac{y^2}{2 \rho} - \frac{1}{2} \ln(2 \pi \rho).
\end{align*}
It can be seen that the natural parameter $\phi$ and its expectation $\mu$ have a linear relationship, i.e., $\phi = \mu$. In the standard linear regression model, the link function is also a linear function, so the standard linear regression model uses the natural link function. At this point, $\phi = \mu = \eta = \bbeta^\top \bx$, and the model prediction value $\widehaty$ is:
\begin{equation}
	\widehaty= \Exp[\ry \mid \bx, \bbeta] = \mu = \eta = \bbeta^\top \bx .
\end{equation}
However, the standard linear regression model may not be appropriate when it is unreasonable to assume normality of the data or when the range of the response variable is restricted. Additionally, in many practical situations, the assumption of homoskedasticity (constant variance) does not hold, further limiting the applicability of the standard linear regression model.

Once again, GLMs extend the standard linear regression framework by relaxing these assumptions. GLMs allow us to choose an appropriate EDF distribution based on the nature of the response variable $y$, and to use a suitable link function that maps the real-valued linear predictor $\eta$ to the domain of $y$. 
\end{example}

\subsection{Grouped Data}\label{section:group_glm}

We have seen that in a GLM, the expected value of the response variable depends only on the natural parameter $\phi$, which in turn is a function of $\bbeta$ and $\bx$ through the linear predictor $\eta = \bbeta^\top \bx$.

In principle, the dispersion $\rho$ could also vary with the input features $\bx$ or otherwise differ from data point to data point. 
In practice, however, it is typically assumed to be constant across all observations: e.g., for Poisson, exponential, Bernoulli, binomial, where $a(\rho) = 1$, or Gaussian, where $a(\rho) = \sigma^2$; see Example~\ref{example:edf_examps}. 
A common exception occurs in grouped data settings, where multiple responses may be observed for the same input $\bx$.

If $ p(y_r \mid  \phi, \rho) $ is an EDF for each $ r \in [1, 2, \ldots, m] $, with natural parameter $ \phi $, log normalizer $ b $, dispersion $ \rho $, and function $ c $. 
Then, the distribution of the average response from the group data,
\begin{equation}
y \triangleq \overline{y} = \frac{1}{m} \sum_r y_r,
\end{equation}
has a probability distribution that is also an EDF. This EDF has the same natural parameter $ \phi $ and log normalizer $ b $ as the original distribution, but $ \rho $ is replaced by $ \frac{\rho}{m} $ and the function $ c $ may be different and a function of $ m $.

Although grouping may not always be feasible when dealing with continuous predictors, it is often beneficial to group data whenever possible, because:
\begin{itemize}
\item it simplifies the equations.
\item it improves speed of convergence and hence computation time.
\item some theory only holds when $ m \gg 1 $.
\end{itemize}

\begin{example}[Grouped logistic regression]\label{example:grouped_logis}
Let $\mu\equiv \pi(\bx) \triangleq \text{sigmoid}(\bbeta^\top\bx)= \frac{e^{\bbeta^\top \bx}}{1 + e^{\bbeta^\top \bx}}$, and suppose that there are several binary values for each $\bx$. 
That is, our data are structured as:
\begin{equation}
\left\{ (\bx_i, \{y_{ir}\}_{r \in [1,2,\ldots,m_i]}) \right\}, \quad \forall\, i\in\{1,2,\ldots,n\},
\end{equation}
where $m_i$ denotes the number of replicates in group $i$, indexed by $r$; $n$ is the total number of groups, and $M = \sum_i m_i$ is the overall sample size.

If our data only consists  of the total counts $\widehat{y}_i \triangleq \sum_{r=1}^{m_i} y_{ir}$,
then although the individual $y_{ir}\sim\bernoullidist(\pi(\bx_i))$, i.e., Bernoulli-distributed with parameter $\pi(\bx_i)$, the sum $\widehat{y}_i$'s follows a  binomial distribution with parameters $m_i$ and $\pi(\bx_i)$ (Exercise~\ref{exercise:bern_binom}), that is:
\begin{equation}
\widehat{\ry}_i = \sum_r \ry_{ir} \sim \binomialdist(m_i, \pi(\bx_i)).
\end{equation}
However, it is often more convenient to model proportions or averages instead of raw counts. Specifically, we define:
\begin{equation}
\ry_i \triangleq \frac{1}{m_i} \bar{\ry}_i = \frac{1}{m_i} \sum_r \ry_{ir} \sim \frac{1}{m_i} \binomialdist(m_i, \pi(\bx_i)),
\end{equation}
where the distribution, corresponding to a binomial variable divided by the number of trials, is known as the \textit{scaled  binomial distribution}; see Exercise~\ref{exercise:scaled_binom}.

The reason for preferring the scaled values $y_i$ over the counts  $\widehat{y}_i$ is that the expectation retains a familiar form:
\begin{equation}
\Exp[\ry \mid  m, \bx] = \frac{1}{m} m \pi(\bx) = \pi(\bx),
\end{equation}
meaning that the expectation function can still be modelled in the same way as for binary regression, using, for example, the logistic function. There is thus no need to include $m$ at the level of the expectation, only in the distribution.
It turns out this scaled binomial distribution also belongs to an EDF distribution:
\begin{align*}
\Pr(y \mid m, \pi) &= \binom{m}{my} \pi^{my} (1 - \pi)^{m - my} \\
&= \exp \left\{ my \ln \pi + (m - my) \ln (1 - \pi) + \ln \binom{m}{my} \right\}  \\
%&= \exp \left\{ m \left( y (\ln \pi - \ln (1 - \pi)) + \ln (1 - \pi) \right) + c(y, \frac{1}{m}) \right\}  \\
&= \exp \left\{ \frac{y \ln \frac{\pi}{1 - \pi} + \ln (1 - \pi)}{\frac{1}{m}} + c(y, \frac{1}{m}) \right\} ,
\end{align*}
which is an EDF with
\begin{align*}
\phi &\triangleq \ln \frac{\pi}{1 - \pi},
\qquad  a(\rho)\triangleq  \frac{1}{m} ,
\qquad b(\phi) \triangleq \ln(1+e^\phi),
\qquad c(y, \frac{1}{m})  \triangleq \ln \binom{m}{my}.
\end{align*}
Thus, the natural parameter $\phi$ is the identical to that in the Bernoulli case (Example~\ref{example:edf_examps}), but $\rho = 1$ is replaced by $\rho = 1/m$. The dispersion thus depends on $m$. This implies that if  different groups $i$ have different numbers of replicates $m_i$, the dispersion will vary accordingly across groups: $\rho_i = \frac{1}{m_i}$.
\end{example}

\section{Model Estimation for GLMs}\label{section:estima_glms}

Suppose we are given a dataset that we would like to model using a GLM. After examining the data---possibly through exploratory data analysis---we have already selected a specific type of GLM that is most appropriate for the dataset (for instance, a Poisson GLM). 
Our next goal is to estimate the model parameters---in particular, to find the value of $\bbeta$ that best fits the data. To do this, we can apply the method of maximum likelihood estimation (MLE) to obtain an estimate $\widehatbbeta$. This estimate can then be used to make predictions:
$$
y_{\new} = \Exp[\ry \mid \bx_{\new}, \widehatbbeta].
$$

One major advantage of the GLM framework is that it allows us to derive a general solution for maximum likelihood estimation that applies to all types of GLMs. This means we don't need to develop separate parameter estimation methods for each specific GLM---such as the Poisson GLM or the exponential GLM. Instead, we can formulate a single, unified solution for the entire class of GLMs, and the estimation procedures for individual models become special cases of this general approach.

In this section, we will begin by introducing the score function and its corresponding score equation for GLMs. We will then explore how to solve the score equation using the Fisher information matrix, often within an iterative algorithm. Additionally, we will examine the theoretical properties of both the score function and the Fisher information matrix.

\paragrapharrow{Likelihood function.}
Consider the grouped data setup where we have predictors and data with possible replicates $\{(\bx_i, \{y_{ir}\}_{r \in [1,2,\ldots,m_i]})\},\ \forall\, {i \in \{1,2,\ldots,n\}}$ (see Section~\ref{section:group_glm}). Recall that under a GLM, given predictors $\{\bx_i\}_{i \in [1,2,\ldots,n]}$, each response $y_{ir_i}$ is independent of the other $y_{ir_j}$, and of the values of all predictors $\bx_j$ with $j \neq i$.
Therefore, the joint probability of the data---that is, the likelihood under the EDF distribution (see \eqref{equation:glm_comp3})---is given by
\begin{equation}\label{equation:glm_likeli}
\mathcalL(\bbeta) = p(\{y_{ir}\} \mid \{\bx_i\}, \bbeta) = p(\{y_{ir}\} \mid \{\phi_i\}, \rho) = \prod_{i=1}^{n} \prod_{r=1}^{m_i} p(y_{ir} \mid  \phi_i, \rho),
\end{equation}
where $p(y_{ir} \mid  \phi_i, \rho) = \exp \big( \frac{y_{ir} \phi_i - b(\phi_i)}{\rho} + c(y_{ir}, \rho) \big)$~\footnote{To abuse the notation a bit, we let $a(\rho)=\rho$ for each $\{\bx_i, y_{ir}\}$ data, the form for $a(\rho)$ will be recovered for each group shortly.} and (see Figure~\ref{fig:glm_rela})
\begin{align*}
\phi_i &= \psi(\mu_i) = \psi\big(h(\eta_i)\big) = \psi\big(h(\bbeta^\top \bx_i)\big),
\qquad \psi \triangleq (b')^{-1}.
\end{align*}

\paragrapharrow{Log-likelihood function.}
The  log-likelihood function is thus given by
\begin{equation}\label{equation:glm_loglikeli}
\begin{aligned}
\ell(\bbeta) &\triangleq \ln \mathcalL(\bbeta) = \ln p(\{y_{ir}\} \mid \{\phi_i\}, \rho) 
= \sum_{i=1}^n \sum_{r=1}^{m_i} \left( \frac{y_{ir} \phi_i - b(\phi_i)}{\rho} + c(y_{ir}, \rho) \right) \\
&\triangleq \sum_{i=1}^n \left(  \frac{y_i \phi_i - b(\phi_i)}{\rho/m_i} + \sum_{r}^{m_i} c(y_{ir}, \rho) \right)
= \sum_{i=1}^n \ell_i,
\end{aligned}
\end{equation}
where we have defined
\begin{align*}
y_i &\triangleq \frac{1}{m_i} \sum_{r} y_{ir},
\qquad  \ell_i \triangleq \frac{y_i \phi_i - b(\phi_i)}{\rho_i} + \sum_{r} c(y_{ir}, \rho), 
\qquad  a(\rho)\triangleq \rho_i \triangleq \frac{\rho }{m_i}.
\end{align*}
Note again that in the non-grouping setting, we have $a(\rho)\equiv \rho$; while in the grouping setting, we have $a(\rho)\equiv \rho_i \triangleq \frac{\rho}{m_i}$ for $i\in\{1,2,\ldots,n\}$; see Remark~\ref{remark:dips_fac_inedf}.
That is, for each group $i\in\{1,2,\ldots,n\}$, the dispersion parameter is $a(\rho) = \frac{\rho}{m_i}$; for each single data $(\bx_i, y_{ir})$ with $i\in\{1,2,\ldots,n\}$ and $r\in\{1,2,\ldots,m_i\}$, the dispersion parameter is $\rho$.

Note that, by definition, \eqref{equation:edf_sumone_mean3}, and \eqref{equation:edf_sumone_var2}, we have 
\begin{equation}\label{equation:mean_var_glm_log}
\Exp[\ry_i \mid  \bbeta, \bx_i] = \mu_i = b'(\phi_i)
\qquad \text{and}\qquad 
\Var[\ry_i \mid \bbeta, \bx_i]  = \rho_i \mathcalV(\mu_i).
\end{equation}

\index{Score function}
\index{Fisher information}
\paragrapharrow{Score function and score equation.}
In statistics, particularly within the context of linear models and GLMs, the \textit{score function} refers to the gradient (i.e., derivative) of the log-likelihood function with respect to the model parameters. It is commonly used to find the maximum likelihood estimates (MLEs), which are the parameter values that maximize the likelihood of observing the given data.
The {score function} under \eqref{equation:glm_likeli} or \eqref{equation:glm_loglikeli} is then given by
\begin{equation}\label{equation:glm_score_func}
(\textbf{Score function}):\qquad 
\begin{aligned}
\sS(\bbeta) 
&= \frac{\partial \ell}{\partial \bbeta} = \sum_i \frac{\partial \ell_i}{\partial \bbeta} = \sum_i \frac{\partial \ell_i}{\partial \phi_i} \frac{\partial \phi_i}{\partial \mu_i} \frac{\partial \mu_i}{\partial \eta_i} \frac{\partial \eta_i}{\partial \bbeta};\\
&= \sum_i \left( \frac{y_i - \mu_i}{\rho_i} \right) \left( \frac{1}{\mathcalV(\mu_i)} \right) h'(\eta_i) \, \bx_i \\
&= \frac{1}{\rho} \sum_i m_i (y_i - \mu_i) \, \frac{1}{\mathcalV(\mu_i)} \, h'(\eta_i) \, \bx_i. 
\end{aligned}
\end{equation}
where, since  $\mu_i = b'(\phi_i)$, $\mathcalV(\mu_i) = b''(\phi_i)$, $\mu_i = h(\eta_i)$, and $\eta_i = \bbeta^\top \bx_i$, we use the facts:~\footnote{Note that here  we assume that $\rho$ does not depend on $\bbeta$.}
\begin{subequations}
\begin{align}
\frac{\partial \ell_i}{\partial \phi_i} &= \frac{y_i - b'(\phi_i)}{\rho_i} = \frac{y_i - \mu_i}{\rho_i}; 
\qquad &\frac{\partial \phi_i}{\partial \mu_i} &= 1 / \left( \frac{\partial \mu_i}{\partial \phi_i} \right) = \frac{1}{b''(\phi_i)} = \frac{1}{\mathcalV(\mu_i)} \label{equation:glm_score2}; \\
\frac{\partial \mu_i}{\partial \eta_i} &= h'(\eta_i) \label{equation:glm_score3}; 
\qquad&\frac{\partial \eta_i}{\partial \bbeta} &= \bx_i. 
\end{align}
\end{subequations}
The maximum likelihood estimate $\widehat{\bbeta}$ must then satisfy the \textit{score equation} (see Proposition~\ref{proposition:fermat_fist_opt}):
\begin{equation}
\textbf{(Score equation)}:\qquad \sS(\widehat{\bbeta}) = \bzero. 
\end{equation}

\begin{remark}[Score equation under EDFs]\label{remark:die_cal}
Note that the dispersion parameter $\rho$ cancels from the score equation, which implies that $\widehat{\bbeta}$ does not depend on $\rho$. 
This is observed in the Gauss-Markov case (Theorem~\ref{theorem:mle-gaussian}), where the MLE for the least squares solution does not depend on the variance parameter $\sigma^2$.
This is another important property of EDFs.

During the process of finding maximum likelihood estimates, we seek the values of $\bbeta$ that make the score function $S(\bbeta)$ equal to zero because these points could be where the log-likelihood function reaches its maximum. However, to ensure that we have found a maximum rather than a minimum, it's also necessary to check the second derivatives or use other methods such as the Fisher information matrix.

In the realm of GLMs, the score equations are typically obtained by setting the score function to zero and solving for the parameters $\bbeta$. Often, these equations do not have closed-form solutions, necessitating the use of numerical methods like the \textit{Newton-Raphson method} or \textit{iteratively reweighted least squares (IRLS)} to find the MLEs; see Section~\ref{section:itera_scor}.

Furthermore, the score function plays a crucial role in assessing the asymptotic properties of estimators. For example, the expected value of the score function is zero, and its variance is related to the Fisher information matrix. These properties help in deriving the asymptotic distribution of the estimator, which is essential for constructing confidence intervals and hypothesis testing.
\end{remark}

\begin{example}[Estimation of GLMs under natural link]\label{example:estimation_naturlink}
For the natural link (Section~\ref{section:glm_natural_cacno_link}), $\phi_i = \eta_i$, so Equations~\eqref{equation:glm_score2} and \eqref{equation:glm_score3} combine to give
$$
\frac{h'(\eta_i)}{\mathcalV(\mu_i)} = \frac{\partial \phi_i}{\partial \mu_i} \frac{\partial \mu_i}{\partial \eta_i} = \frac{\partial \phi_i}{\partial \eta_i} = 1.
$$
As a result, the score function  simplifies to
$
\sS(\bbeta) = \frac{1}{\rho} \sum_i m_i (y_i - \mu_i)  \bx_i. 
$
\end{example}

\paragrapharrow{Fisher information.}

To solve the score equation, we will also need the second derivative of the log-likelihood. Its negative is called the \textit{observed Fisher information}, defined as
\begin{equation}\label{equation:def_expfish_info}
\sI_{\mathrm{obs}}(\bbeta) \triangleq - \frac{\partial^2 \ell}{\partial \bbeta \partial \bbeta^\top} = - \frac{\partial \sS}{\partial \bbeta}.
\end{equation}
Note that, at the MLE, $\sI_{\text{obs}}(\widehat{\bbeta})$ is positive semidefinite by definition. Because it is a function of the data $\{y_i\}$, $\sI_{\text{obs}}$ has a probability distribution. In practice, the observed Fisher information is often approximated by the  \textit{Fisher information} (a.k.a., the \textit{expected Fisher information}; see Section~\ref{section:mvu_esti}):
\begin{equation}\label{equation:def_fish_info}
	\sI_n(\bbeta) = \Exp\left[-\frac{\partial \sS}{\partial \bbeta}\right],
\end{equation}
where the expectation is taken over the joint probability distribution of the data $p(\{y_{ir_i}\}\mid \bbeta, \{\bx_i\})$.

\begin{example}[Poisson GLM]\label{example:poiss_twocase}
Suppose $m_i=1$ with the Poisson GLM. We have $\rho=1$ by Example~\ref{example:edf_examps}. 
Let the Poisson parameter be $\lambda(\bx_i,\bbeta)$ for each data $(\bx_i,y_i)$ such that $\ry_i \mid \bx_i, \bbeta \sim \poissondist(\lambda(\bx_i, \bbeta))$
\paragraph{Natural link.}
For the natural link such that $h(\eta)=b'(\eta)=e^\eta$ by Example~\ref{example:edf_examps}, we have that: $\lambda(\bx_i, \bbeta) = \mu(\bx_i, \bbeta) = h(\eta(\bx_i, \bbeta)) = e^{\eta(\bx_i, \bbeta)} = e^{\bbeta^\top \bx_i}$.
By Example~\ref{example:estimation_naturlink}, we have
$\sS(\bbeta) = \sum_i (y_i - e^{\bbeta^\top \bx_i}) \, \bx_i$;
using \eqref{equation:def_expfish_info}, the observed Fisher information is:
$$
\sI_{\text{obs}}(\bbeta) = \sum_i e^{\bbeta^\top \bx_i} \, \bx_i \bx_i^\top.
$$
Since $\sI_{\text{obs}}(\bbeta)$ does not depend on $\ry_i$,  the expected Fisher information coincides with the observed Fisher information $\sI_{\text{obs}}(\bbeta)$.

\paragraph{Identity link.}
Now consider the identity link, where $h(\eta) = \eta$. In this case:
$$
\lambda(\bx, \bbeta) = \mu(\bx, \bbeta) = h(\eta(\bx, \bbeta)) =  \bbeta^\top \bx, 
\qquad
\mathcalV(\mu)=b''\big((b')^{-1}(\mu)\big) = \mu, 
\qquad
h'(\eta) = 1.
$$
By~\eqref{equation:glm_score_func} and \eqref{equation:def_expfish_info}, the score function and the observed Fisher information become: 
\begin{align*}
\sS(\bbeta) &= \sum_i (y_i - \mu_i) \frac{1}{\mu_i} \bx_i 
= \sum_i (y_i - \bbeta^\top \bx_i) \frac{1}{\bbeta^\top \bx_i} \bx_i 
= \sum_i \left( \frac{y_i}{\bbeta^\top \bx_i} - 1 \right) \bx_i;\\
\sI_{\text{obs}}(\bbeta) &= \sum_i \frac{y_i}{(\bbeta^\top \bx_i)^2} \bx_i \bx_i^\top.
\end{align*}
The expected Fisher information is:
$$
\begin{aligned}
\sI_n(\bbeta) &= \Exp[\sI_{\text{obs}}(\bbeta)]
= \Exp \left[ \sum_i \frac{\ry_i}{(\bbeta^\top \bx_i)^2} \bx_i \bx_i^\top \right] 
= \sum_i \frac{\Exp[\ry_i \mid  \bbeta, \bx_i]}{(\bbeta^\top \bx_i)^2} \bx_i \bx_i^\top 
%= \sum_i \frac{\bbeta^\top \bx_i}{(\bbeta^\top \bx_i)^2} \bx_i \bx_i^\top 
= \sum_i \frac{1}{\bbeta^\top \bx_i} \bx_i \bx_i^\top.
\end{aligned}
$$
Note that $\sI_n(\bbeta) \neq \sI_{\text{obs}}(\bbeta)$ in this case.
\end{example}

\index{Fisher information}
\index{Score function}
\index{Score equation}
\subsection{Properties of Score Function and Fisher Information}

Having introduced the score function $\sS(\bbeta)$ and the Fisher information $\sI_n(\bbeta)$, we now explore some of their key statistical properties. 
Define $\sS_i(\bbeta) \triangleq \frac{\partial \ell_i}{\partial \bbeta}$, where $\ell_i$ denotes the log-likelihood for data $i$ with $i\in\{1,2, \ldots,n\}$. Then, the total score function can be written as: $\sS(\bbeta) = \sum_i \sS_i(\bbeta)$.

\paragrapharrow{Moments of $\sS(\bbeta)$.}
Since $\Exp[\ry_i \mid  \bbeta, \bx_i] = \mu_i$ by \eqref{equation:mean_var_glm_log}, the expectation of $\sS(\bbeta)$ can be computed from \eqref{equation:glm_score_func}  as:
\begin{equation}
\Exp[\sS(\bbeta)] = \sum_i \Exp[\sS_i(\bbeta)] 
= \sum_i \frac{\Exp[\ry_i \mid  \bbeta, \bx_i] - \mu_i}{\rho_i} \frac{1}{\mathcalV(\mu_i)} h'(\eta_i) \bx_i 
= \bzero.
\end{equation}
Using \eqref{equation:glm_score_func} again, we can calculate the variance of $\sS(\bbeta)$ as:
\begin{equation}\label{equation:varofsbeta}
\small
\begin{aligned}
\Cov[\sS(\bbeta)] 
&= \sum_i \Cov[\sS_i(\bbeta)] 
= \sum_i \Cov \left[ \frac{h'(\eta_i)}{\rho_i \mathcalV(\mu_i)} \bx_i (\ry_i - \mu_i) \right] \\
&\stackrel{\dag}{=} \sum_i \left( \frac{h'(\eta_i)}{\rho_i \mathcalV(\mu_i)} \bx_i \right) \Var[\ry_i - \mu_i] \left( \frac{h'(\eta_i)}{\rho_i \mathcalV(\mu_i)} \bx_i^\top \right)
= \sum_i \left( \frac{h'(\eta_i)^2}{\rho_i^2 \mathcalV(\mu_i)^2} \bx_i \bx_i^\top \right) \Var[\ry_i]\\
&= \sum_i \frac{h'(\eta_i)^2}{\rho_i \mathcalV(\mu_i)} \bx_i \bx_i^\top,
\end{aligned}
\end{equation}
where the equality ($\dag$) follows from the fact that $\Cov[\bA\rmX, \bB\rmY] = \bA\Cov[\rmX, \rmY] \bB^\top$ for any random matrices $\rmX, \rmY$ and fixed matrices $\bA, \bB$ with appropriate dimensions, and the last equality follows from the fact that $\Var[\ry_i \mid \bbeta, \bx_i] = \rho_i \mathcalV(\mu_i)$ by \eqref{equation:mean_var_glm_log}.

\paragrapharrow{Property of $\sI_n(\bbeta)$.}
Recall from the definition of the (expected) Fisher information $\sI_n(\bbeta) = \Exp\left[ -\frac{\partial \sS}{\partial \bbeta} \right]$ by \eqref{equation:def_fish_info}.
Taking the derivative of $\sS(\bbeta)$ in \eqref{equation:glm_score_func} and applying the product rule for derivatives, we have:
\begin{equation}
\begin{aligned}
\frac{\partial \sS}{\partial \bbeta} &= \frac{1}{\rho} \sum_i m_i \bx_i \left( \frac{\partial}{\partial \bbeta} (y_i - \mu_i) \cdot \frac{h'(\eta_i)}{\mathcalV(\mu_i)} + (y_i - \mu_i) \cdot \frac{\partial}{\partial \bbeta} \left( \frac{h'(\eta_i)}{\mathcalV(\mu_i)} \right) \right) \\
&= \sum_i \frac{\bx_i}{\rho_i} \left( -\frac{\partial \mu_i}{\partial \bbeta} \cdot\frac{h'(\eta_i)}{\mathcalV(\mu_i)} + (y_i - \mu_i)\cdot \frac{\partial}{\partial \bbeta} \left( \frac{h'(\eta_i)}{\mathcalV(\mu_i)} \right) \right).
\end{aligned}
\end{equation}
Therefore, the expected Fisher information is 
\begin{equation}\label{equation:propf_fbeta}
\begin{aligned}
\sI_n(\bbeta) 
&= \Exp\left[-\frac{\partial \sS}{\partial \bbeta}\right] 
= \sum_i \frac{\bx_i}{\rho_i} \left( \frac{\partial \mu_i}{\partial \bbeta} \frac{h'(\eta_i)}{\mathcalV(\mu_i)} 
- \Exp[\ry_i - \mu_i] \frac{\partial}{\partial \bbeta} \left( \frac{h'(\eta_i)}{\mathcalV(\mu_i)} \right) \right) \\
&
%= \sum_i \frac{\bx_i}{\rho_i} \left(\frac{\partial \mu_i}{\partial \bbeta} \frac{h'(\eta_i)}{\mathcalV(\mu_i)} \right)
\stackrel{\dag}{=} \sum_i \frac{\bx_i}{\rho_i} \left( h'(\eta_i) \bx_i^\top \frac{h'(\eta_i)}{\mathcalV(\mu_i)} \right) 
= \sum_i \frac{h'(\eta_i)^2}{\rho_i \mathcalV(\mu_i)} \bx_i \bx_i^\top,
\end{aligned}
\end{equation}
where the equality ($\dag$) follows from the fact that $\Exp[\ry_i - \mu_i] = 0$ and 
$\frac{\partial \mu_i}{\partial \bbeta} = \frac{\partial \mu_i}{\partial \eta_i} \frac{\partial \eta_i}{\partial \bbeta} = h'(\eta_i) \bx_i^\top$.
Noting the expression of \eqref{equation:varofsbeta}, we conclude that:
\begin{equation}
\sI_n(\bbeta) = \Cov[\sS(\bbeta)],
\end{equation}
which matches the definition of the Fisher information given previously in Equation~\eqref{equation:fish_mul_defi}.

\begin{example}[Score function and Fisher information under natural link]\label{example:score_fisher_naturlink}
Following Example~\ref{example:estimation_naturlink}, for the natural link, we have that $\frac{h'(\eta_i)}{\mathcalV(\mu_i)} =\frac{\partial \phi_i}{\partial \eta_i} = 1$. 
Therefore,
\begin{equation}
\sS(\bbeta) = \frac{1}{\rho} \sum_i m_i (y_i - \mu_i)  \bx_i = \sum_i \frac{1}{\rho_i} \big(y_i - h(\eta_i)\big) \bx_i.
\end{equation}
Let $\sS_i \triangleq \frac{1}{\rho_i} \big(y_i - h(\eta_i)\big) \bx_i$ such that $\sS = \sum_i \sS_i$. We then have
\begin{align}
\sI_{\text{obs}}(\bbeta) &= -\frac{\partial \sS}{\partial \bbeta} = -\sum_i \frac{\partial \sS_i}{\partial \bbeta} = -\sum_i \frac{\partial \sS_i}{\partial \eta_i} \frac{\partial \eta_i}{\partial \bbeta} = \sum_i \frac{h'(\eta_i)}{\rho_i} \bx_i \bx_i^\top ;\\
\sI_n(\bbeta) &= \Cov[\sS(\bbeta)] = \sum_i \frac{h'(\eta_i)}{\rho_i} \bx_i \bx_i^\top.
\end{align}
Thus, for the natural link case, we observe that $\sI_n(\bbeta) = \sI_{\text{obs}}(\bbeta)$.
This extends the results in the Poisson GLM with natural link in Example~\ref{example:poiss_twocase}.
\end{example}

\subsection{Iterative Solution of Score Equation}\label{section:itera_scor}

So far, we have examined  how to formulate the score equation for obtaining maximum likelihood estimates. We have also established several  important properties of both the score equation and Fisher information matrix. 
Now, we turn our attention to the practical challenge of solving the score equation. 
In general, it is not possible to find an exact closed-form solution to the score equation, except in special cases---such as in the Gauss-Markov linear model discussed in Chapter~\ref{sec:lr-gaussian-noise}. Therefore, we must rely on numerical methods that can be implemented using computational tools.

Generally, we have two primary options: directly optimizing the log-likelihood function $\ell$, or solving the score equation. Numerous algorithms are available for these tasks. In this context, our focus will be on one particular method: 
\textit{iteratively reweighted least squares (IRLS)}, which is also known as \textit{iterative weighted least squares (IWLS)}.

\subsection*{Matrix Notations}

To derive the iterative method for calculating the maximum likelihood estimate of $\bbeta$ using the score equation, it is helpful to express the problem using matrix notation. Below are the key definitions:
\begin{itemize}
\item Let $\rvy \in \real^n$ be the random vector whose components are $\ry_i$, the observed response values. 
\item Let $\bX =[x_{ij}]\in \real^{n \times p}$ be the data matrix. To abuse the notation, we let each row of $\bX$ be $\bx_i$ for convenience; note that the $i$-th row of a matrix is denoted as $\bx^{(i)}$ otherwise in this book.
\item Let $\bmu \in \real^n$ be the vector whose components are $\mu_i = h(\bbeta^\top \bx_i)$, so that $\bmu = \Exp[\rvy]$.
\item Let $\bD =[d_{ij}] \in \real^{n \times n}$ be the diagonal matrix with entries $d_{ii} = h'(\eta_i)$.
\item Let $\bOmega=[\omega_{ij}] \in \real^{n \times n}$ be the covariance matrix of $\rvy$, with entries:
$$
\omega_{ij} = \Cov[\ry_i, \ry_j] = \Var[\ry_i] \delta_{ij} = \rho_i \mathcalV(\mu_i) \delta_{ij}.
$$
That is, 
$$
\bOmega = \diag(\Var[\ry_1] , \Var[\ry_2], \ldots, \Var[\ry_n] ) = \diag(\rho_1 \mathcalV(\mu_1), \rho_2 \mathcalV(\mu_2), \ldots, \rho_n \mathcalV(\mu_n)).
$$
\item Let $\bG =[g_{ij}] \in \real^{n \times n}$ be the diagonal matrix with components $g_{ii} =m_i $, known as the \textit{grouping matrix}.
\end{itemize}

\paragrapharrow{Score function and Fisher information.}
Recall that the score function \eqref{equation:glm_score_func} and Fisher information \eqref{equation:propf_fbeta} are defined as 
$$
\sS(\bbeta) = \sum_i \left( \frac{y_i - \mu_i}{\rho_i \mathcalV(\mu_i)} \right) h'(\eta_i) \bx_i
\qquad\text{and}\qquad 
\sI_n(\bbeta) = \sum_i \frac{h'(\eta_i)^2}{\rho_i \mathcalV(\mu_i)} \bx_i \bx_i^\top.
$$
In terms of the matrix notation, these become
\begin{align}\label{equation:matnot_sc_fis}
\sS(\bbeta) = \bX^\top \bD \bOmega^{-1} (\by - \bmu)
\qquad\text{and}\qquad 
\sI_n(\bbeta) = \bX^\top \bD^\top \bOmega^{-1} \bD \bX.
\end{align}

\paragrapharrow{Natural link.}
Note that for the natural link case (see Section~\ref{section:glm_natural_cacno_link}, Examples~\ref{example:estimation_naturlink}, \ref{example:poiss_twocase} and \ref{example:score_fisher_naturlink}),
it holds that $ \frac{\partial \phi_i}{\partial \eta_i} = \frac{h'(\eta_i)}{\mathcalV(\mu_i)} = 1$.
Thus, with the definition of $\rho_i = \rho / m_i$, we have:
$$
h'(\eta_i) = \mathcalV(\mu_i) = \frac{\Var[\ry_i]}{\rho_i} = m_i \frac{\Var[\ry_i]}{\rho}, \quad \forall \,  i\in\{1,2,\ldots,n\}
$$
Therefore, it follows that 
$$
\bD = \frac{1}{\rho} \bG \bOmega = \frac{1}{\rho} \bOmega \bG,
$$
whence we have 
\begin{equation}
\sS(\bbeta) = \frac{1}{\rho} \bX^\top \bG (\by - \bmu)
\qquad\text{and}\qquad 
\sI_n(\bbeta) = \frac{1}{\rho^2} \bX^\top \bG^\top \bOmega \bG \bX.
\end{equation}

\index{Fisher scoring method}
\index{Newton-Raphson method}
\subsection*{Iterative Solution of Score Equation}
\paragrapharrow{Newton-Raphson and Fisher scoring methods.}
We start by recalling the Newton-Raphson method,~\footnote{See, for example, \citet{lu2025practical} for more details.} which is commonly used to find the root (zero) of a function. 
In our case, we aim to solve the score equation:
\begin{equation}
\sS(\widehat{\bbeta}) = \bzero.
\end{equation}

Generally, for an iterative process for solving the score equation, we generate a series of parameters that converge to some point.
Denoting $t=1,2,\ldots$ as the iteration number,  iterative methods generate a sequence of vectors:
$$
\bbeta^{(1)}, \bbeta^{(2)}, \ldots, \bbeta^{(T)},
$$
where each new estimate is based on the previous one.
We then focus on the iteration $t$.
By the linear approximation theorem (Theorem~\ref{theorem:linear_approx}), we can approximate $\sS$ linearly at some point $\bbeta^\toptzero$:
$$
\sS(\bbeta^\toptzero + \bd^\toptzero) = \sS(\bbeta^\toptzero) + \frac{\partial \sS(\bbeta^\toptzero)}{\partial \bbeta} \bd^\toptzero + \mathcalO\big(\normtwobig{\bd^\toptzero}^2\big).
$$
Since we want $\sS(\bbeta^\toptzero + \bd^\toptzero) = \bzero$, ignoring the third term of the above equality, we can obtain $\bd^\toptzero$ approximately by:
\begin{equation}\label{equation:score_bd}
\begin{aligned}
\frac{\partial \sS(\bbeta^\toptzero)}{\partial \bbeta} \bd^\toptzero = -\sS(\bbeta^\toptzero)
&\quad \iff \quad  \\
\sI_{\text{obs}}(\bbeta^\toptzero) \bd^\toptzero = \sS(\bbeta^\toptzero)
&\quad \iff \quad
\bd^\toptzero = \big(\sI_{\text{obs}}(\bbeta^\toptzero)\big)^{-1} \sS(\bbeta^\toptzero),
\end{aligned}
\end{equation}
where we use the fact that $-\frac{\partial \sS(\bbeta^\toptzero)}{\partial \bbeta} = \sI_{\text{obs}}(\bbeta^\toptzero)$ by \eqref{equation:def_expfish_info}.
This gives us the updated estimate for the next iteration:
\begin{equation}
\bbeta^\toptone \leftarrow  \bbeta^\toptzero + \bd^\toptzero = \bbeta^\toptzero + \big(\sI_{\text{obs}}(\bbeta^\toptzero)\big)^{-1} \sS(\bbeta^\toptzero).
\end{equation}

\begin{algorithm}[h] 
\caption{Newton-Raphson or Fisher Scoring Methods for Solving Score Equation}
\label{alg:newton_raph_sco}
\begin{algorithmic}[1] 
\Require Score function $\sS(\bbeta)$; 
\State {\bfseries Input:}  Initialize $\bbeta^{(1)}$;
\For{$t=1,2,\ldots$}
\State Find a  direction $\bd^\toptzero$ by \eqref{equation:score_bd} or \eqref{equation:fishing_score_bd};
%\State Pick a stepsize $\eta_t$;
\State $\bbeta^{(t+1)} \leftarrow \bbeta^\toptzero +  \bd^\toptzero$;
\EndFor
\State {\bfseries Return:}  final  $\bbeta^{(t)}$;
\end{algorithmic} 
\end{algorithm}

Because computing $\sI_{\text{obs}}$ and  its inverse can be difficult in practice, we often replace it with the expected Fisher information \eqref{equation:def_fish_info}. 
This modified approach is known the \textit{Fisher scoring method}, where we compute $\bd^\toptzero$ by:
\begin{equation}\label{equation:fishing_score_bd}
\bd^\toptzero = \big(\sI_{\text{obs}}(\bbeta^\toptzero)\big)^{-1} \sS(\bbeta^\toptzero)
\qquad \implies\qquad 
\bd^\toptzero \approx \big(\sI_n(\bbeta^\toptzero)\big)^{-1} \sS(\bbeta^\toptzero).
\end{equation}
The complete  procedure for this method is summarized  in Algorithm~\ref{alg:newton_raph_sco}.

\index{Iterative weighted least squares (IWLS)}
\index{Iteratively reweighted least squares (IRLS)}
\paragrapharrow{Iteratively reweighted least squares (IRLS).}
We will now use \eqref{equation:fishing_score_bd} to derive the \textit{iteratively reweighted least squares (IRLS)} method in matrix form. 
From \eqref{equation:fishing_score_bd} and the iterative update rule $\bbeta^{(t+1)} \leftarrow \bbeta^\toptzero +  \bd^\toptzero$, we have that
\begin{equation}\label{equation:irls_upr1}
	\sI_n(\bbeta^\toptzero) \bd^\toptzero = \sS(\bbeta^\toptzero)
	\qquad \iff \qquad 
	\sI_n(\bbeta^\toptzero) \bbeta^\toptone = \sI_n(\bbeta^\toptzero) \bbeta^\toptzero + \sS(\bbeta^\toptzero).
\end{equation}
Using the previously defined matrix notations and defining:
\begin{equation}\label{equation:weight_def}
\bW \triangleq \bD^\top \bOmega^{-1} \bD
\end{equation}
($\bD, \bOmega$, and $\bW$ are all diagonal matrices), by \eqref{equation:matnot_sc_fis}, we can express the score function and Fisher information as follows:
\begin{align}
\sS(\bbeta) &= \bX^\top \bD \bOmega^{-1} (\by - \bmu)  = \bX^\top \bW \bD^{-1} (\by - \bmu);\\
\sI_n(\bbeta) &= \bX^\top \bD^\top \bOmega^{-1} \bD \bX = \bX^\top \bW \bX.
\end{align}
Let $\sS^\toptzero \triangleq\sS(\bbeta^\toptzero)$, $\sI^\toptzero \triangleq\sI_n(\bbeta^\toptzero)$, $\bD^\toptzero = \diag(h'(\bbeta^\toptzeroTOP\bx_1), h'(\bbeta^\toptzeroTOP\bx_2), \ldots, h'(\bbeta^\toptzeroTOP\bx_n))$, and $\bmu^\toptzero = [h(\bbeta^\toptzeroTOP\bx_1), h(\bbeta^\toptzeroTOP\bx_2), \ldots, h(\bbeta^\toptzeroTOP\bx_n)]^\top$. Then, \eqref{equation:irls_upr1} can be denoted as 
\begin{align}
(\bX^\top \bW^\toptzero \bX) \bbeta^\toptone
&=
\sI^\toptzero \bbeta^\toptzero + \sS^\toptzero 
= \bX^\top \bW^\toptzero \bX \bbeta^\toptzero + \bX^\top \bW^\toptzero \bD^{\toptzero-1} (\by - \bmu^\toptzero) \\
&\triangleq  \bX^\top \bW^\toptzero \widetilde{\by}^\toptzero\\
&\implies \qquad 
\bbeta^\toptone = \big(\bX^\top \bW^\toptzero \bX\big)^{-1} \bX^\top \bW^\toptzero \widetilde{\by}^\toptzero, \label{equation:irls_rule}
\end{align}
where $\widetilde{\by}^\toptzero \triangleq \bX \bbeta^\toptzero + \bD^{\toptzero-1} (\by - \bmu^\toptzero)$.

Thus, to find a solution for $\sS(\bbeta) = \bzero$, we can start from an arbitrary point $\bbeta^\topone$ and iteratively apply \eqref{equation:irls_rule} until a convergence criterion is met; e.g., the criterion defined by \eqref{equation:des_stopcri1}.
This sequence of iterated operations is called \textit{iteratively reweighted least squares (IRLS)} or \textit{iterative weighted least squares (IWLS)} since each iteration is the solution to the following weighted least squares problem: minimize the quantity $\ell_t(\bbeta)$ with respect to $\bbeta$, where
\begin{equation}
\ell_t(\bbeta) \triangleq (\widetilde{\by}^\toptzero - \bX\bbeta)^\top \bW^\toptzero (\widetilde{\by}^\toptzero - \bX\bbeta)
\end{equation}
and $\bW^\toptzero$ is known as the \textit{weight matrix}; see Section~\ref{section:generalizedLS}.
The full procedure of IRLS is formulated in Algorithm~\ref{alg:irls}.

\begin{algorithm}[h] 
\caption{IRLS for Solving Score Equation}
\label{alg:irls}
\begin{algorithmic}[1] 
\Require Score function $\sS(\bbeta)$; 
\State {\bfseries Input:}  Initialize $\bbeta^{(1)}$;
\For{$t=1,2,\ldots$}
\State Compute weight matrix $\bW^\toptzero$ and $\widetilde{\by}^\toptzero$;
\State $\bbeta^\toptone \leftarrow  \big(\bX^\top \bW^\toptzero \bX\big)^{-1} \bX^\top \bW^\toptzero \widetilde{\by}^\toptzero$;
\EndFor
\State {\bfseries Return:}  final  $\bbeta^{(t)}$;
\end{algorithmic} 
\end{algorithm}

\subsection{Estimation of Dispersion Parameter}
Since the dispersion parameter $\rho$ cancels out in the score equation $\sS(\widehat{\bbeta}) = \bzero$ (Remark~\ref{remark:die_cal}), it is not necessary to estimate $\rho$ in order to estimate the coefficients $\bbeta$. 
However, the variance of the estimator $\Var[\widehat{\bbeta}]$ does depend on $\rho$; for example, the variance of $\widehatbbeta$ in the Gauss-Markov linear model depends on the variance $\sigma^2$, as shown in Theorem~\ref{theorem:gauss-markov}. 

Therefore, if needed or of interest, we can estimate $\rho$ using the following formula:
\begin{equation}\label{equation:est_disp}
\widehat{\rho} = \frac{1}{n - p} \sum_i m_i \frac{(y_i - \widehat{\mu}_i)^2}{\mathcalV(\widehat{\mu}_i)},
\end{equation}
where $p$ is the number of parameters of the model, i.e., the number of columns in $\bX\in\real^{n\times p}$, and $\widehat{\mu}_i = \widehatbbeta^\top\bx_i$. 
This estimator is motivated by the fact that:
\begin{equation}
\Var[\ry_i] = \Exp[(\ry_i - \mu_i)^2] = \rho_i \mathcalV(\mu_i) = \frac{\rho}{m_i} \mathcalV(\mu_i)
\quad \implies \quad 
\rho= \Exp\left[m_i \frac{(\ry_i - \mu_i)^2}{\mathcalV(\mu_i)}\right].
\end{equation}
Hence, once we have obtained an estimate $\widehat{\bbeta}$, we can use its value and \eqref{equation:est_disp} to estimate $\widehat{\rho}$. 
The division by  $n-p$ ensures unbiased estimation, similar to the reasoning in Theorem~\ref{theorem:ss-chisquare}.

\begin{example}
For the Gaussian distribution, when $\ry\mid \bbeta, \bx \sim \normal(\mu, \sigma^2)$ with $m_i = 1$, we have $\mathcalV(\mu_i) = 1$ and thus,
\begin{equation}
	\widehat{\rho} = \frac{1}{n - p} \sum_i (y_i - \widehat{\mu}_i)^2 = S^2.
\end{equation}
This coincides with the well-known unbiased estimator of the error variance in the Gauss-Markov linear model; see Theorem~\ref{theorem:ss-chisquare}.
\end{example}

\subsection{Prediction}

Assume we have fitted a GLM, resulting in an estimated parameter vector $\widehat{\bbeta} \in \real^p$. 
The predicted value $\widehatby\equiv\widehat{\bmu}$ of the observed response $\by$ is then given by:
\begin{equation}
\widehatby = \Exp[\rvy \mid  \widehat{\bbeta}, \bX] = h(\widehat{\boldsymbol{\eta}}) = h(\bX\widehat{\bbeta} ) \equiv\widehat{\bmu}.
\end{equation}
The linear predictor $\widehat{\boldsymbol{\eta}}$ is a vector where each element represents the predicted linear combination of the predictors for each observation.

When new data become available, such as a new predictor vector $\bx_{\new}$, the GLM can be used to make predictions for these new observations as well. The first step is to compute the corresponding linear predictor:
$$
\widehat{\eta}_{\new } = \widehat{\bbeta}^\top \bx_{\new}.
$$
Once $\widehat{\eta}_{\new }$ is obtained, the predicted response for $\bx_{\new}$ is found  by using the response function $h(\cdot)$:
\begin{equation}
	\widehat{y}_{\new} = \Exp[\ry \mid  \widehat{\bbeta}, \bx_{\new}] = h(\widehat{\eta}_{\new}) = h(\widehat{\bbeta}^\top \bx_{\new}).
\end{equation}
Combining \eqref{equation:est_disp} and \eqref{equation:edf_sumone_var2} shows that the variance of this prediction is 
\begin{equation}
\Var[\ry_{\new} \mid \bx_{\new}, \widehatbbeta] = \widehat{\rho} \mathcalV(h(\widehat{\bbeta}^\top \bx_{\new})).
\end{equation}

\index{Asymptotics}
\section{Asymptotics for GLMs*}
In the previous section, we have seen how to fit a GLM to a dataset and estimate its parameters $\widehat{\bbeta}$. In this section, we will briefly explore the asymptotic properties of GLMs, learn how to make predictions, compute confidence intervals and confidence regions, the results of which can be used in the hypothesis testing procedures (see Section~\ref{section:mklm_hypotest}).

\subsection{Asymptotic Properties of $\widehat{\bbeta}$}
We have discussed large-sample properties of the OLS estimator in Sections~\ref{section:asymp_ols}. For now, we briefly discuss the asymptotic properties of the $\widehat{\bbeta}$ estimator in the context of GLMs \citep{fahrmeir1985consistency}.
In this group setup (Section~\ref{section:group_glm}), the term \textit{asymptotic} refers to the scenario where  $M = \sum_{i=1}^n m_i \to \infty$. This can occur if $n \to \infty$, or if each $m_i \to \infty$, or through some combination of both.

Let  $\bbeta_0$  denote the true value of $\bbeta$.
In what follows, we assume that $\widehat{\bbeta}$ is a consistent estimator of $\bbeta_0$, i.e., $\widehat{\bbeta}$ converges in probability to $\bbeta_0$, meaning that $\Pr\big(\normtwobig{\widehat{\bbeta} - \bbeta_0} \geq \epsilon\big) \to 0$ as $n \to \infty$ for any $\epsilon>0$ (Definition~\ref{definition:consist}). 
Recall that  this is denoted by $\widehat{\bbeta} \stackrel{p}{\longrightarrow} \bbeta_0$; see Definition~\ref{definition:convg_prob}. We will also abuse this notation to mean \textit{``tends to asymptotically'' for expectations}, i.e., if we write $\Exp[\rx] \overset{a}{=} x$, that means $\Exp[\rx] \xrightarrow{n \to \infty} x$.

Given the assumption of consistency, $\widehat{\bbeta}$ will be close to $\bbeta_0$ in large samples. Thus, we can expand the score function 
$\sS$ around it by Theorem~\ref{theorem:linear_approx}:
\begin{equation}\label{equation:asyump_hatbeta}
\begin{aligned}
\sS(\widehat{\bbeta}) = \bzero \overset{a}{=} \sS(\bbeta_0) + \frac{\partial \sS(\bbeta_0)}{\partial \bbeta} (\widehat{\bbeta} - \bbeta_0)
= \sS(\bbeta_0) - \sI_{\text{obs}}(\bbeta_0) (\widehat{\bbeta} - \bbeta_0)\\
\qquad \implies \qquad  \widehat{\bbeta} - \bbeta_0 \overset{a}{=} \sI_{\text{obs}}(\bbeta_0)^{-1} \sS(\bbeta_0).
\end{aligned}
\end{equation}

\paragrapharrow{Fisher scoring method.}
Previously, we stated that we often use the Fisher information in place of the observed Fisher information (known as the Fisher scoring method; see Algorithm~\ref{alg:newton_raph_sco}). 
In asymptotic analyses, this substitution is generally acceptable. We can heuristically justify this by considering the behavior of the observed Fisher information matrix. Specifically, for any $\bbeta$, we have:
\begin{equation}
\frac{1}{n} \sI_{\text{obs}}(\bbeta) = -\frac{1}{n} \frac{\partial \ell}{\partial \bbeta \partial \bbeta^\top}(\bbeta) = -\frac{1}{n} \sum_{i=1}^n \frac{\partial \ell_i}{\partial \bbeta \partial \bbeta^\top}(\bbeta) 
\to -\Exp\left[\frac{\partial \ell_i}{\partial \bbeta \partial \bbeta^\top}(\bbeta)\right] = \sI_1(\bbeta),
\end{equation}
where $ \sI_1(\bbeta) $ is the expected Fisher information for a single observation, and the convergence follows from the law of large numbers as $ n \to \infty $ (Definition~\ref{theorem:l2weaklaw_large}). It can be shown  that $ \sI_n(\bbeta) = n \sI_1(\bbeta) $, thus justifying use of $ \sI_{\text{obs}}(\bbeta) \overset{a}{=} \sI_n(\bbeta) $ in the asymptotic arguments that follow.

\paragrapharrow{Mean of $\widehat{\bbeta}$.}
From \eqref{equation:asyump_hatbeta}, we have:
$$
\widehat{\bbeta} - \bbeta_0 \overset{a}{=} \sI_{\text{obs}}(\bbeta_0)^{-1} \sS(\bbeta_0) \overset{a}{=} \sI_n(\bbeta_0)^{-1} \sS(\bbeta_0).
$$
Because convergence in probability implies convergence in distribution, this in turn implies that
$$
\Exp[\widehat{\bbeta} - \bbeta_0] \overset{a}{=} \sI_n(\bbeta_0)^{-1} \Exp[\sS(\bbeta_0)] = \bzero.
$$
In other words, the estimator $ \widehat{\bbeta} $ is asymptotically unbiased.

\paragrapharrow{Variance of $\widehat{\bbeta}$.}
Since $ \Exp[\widehat{\bbeta} - \bbeta_0] \overset{a}{=} \bzero $, we can compute its covariance as follows:
$$
\begin{aligned}
\Cov[\widehat{\bbeta} - \bbeta_0] &\overset{a}{=} \Exp[(\widehat{\bbeta} - \bbeta_0)(\widehat{\bbeta} - \bbeta_0)^\top]
\overset{a}{=} \Exp[\sI_n(\bbeta_0)^{-1} \sS(\bbeta_0) \sS(\bbeta_0)^\top \sI_n(\bbeta_0)^{-\top}] \\
&= \sI_n(\bbeta_0)^{-1} \Exp[\sS(\bbeta_0) \sS(\bbeta_0)^\top] \sI_n(\bbeta_0)^{-\top}
= \sI_n(\bbeta_0)^{-1} \Cov[\sS(\bbeta_0)] \sI_n(\bbeta_0)^{-\top} \\
&= \sI_n(\bbeta_0)^{-1},
\end{aligned}
$$
where we use the fact that $ \sI $ is symmetric and  that $ \sI_n(\bbeta_0) = \Cov[\sS(\bbeta_0)] $; see \eqref{equation:fish_mul_defi}.
Therefore, this concludes that 
\begin{equation}\label{equation:var_betahat}
\Cov[\widehat{\bbeta}] = \Cov[\widehat{\bbeta} - \bbeta_0] \overset{a}{=} \sI_n(\bbeta_0)^{-1}.
\end{equation}

\paragrapharrow{Asymptotic normality.}
The following is a sketch of the argument of asymptotic normality for $ \widehat{\bbeta} - \bbeta_0 $, i.e., $ \widehat{\bbeta} - \bbeta_0 $ converges asymptotically to a Gaussian distribution. We begin with the expression:
\begin{equation}
\sS(\bbeta) = \sum_i \sS_i(\bbeta)
\end{equation}
where $ \sS_i(\bbeta) $ is defined as $\sS_i(\bbeta) \triangleq \frac{\partial \ell_i}{\partial \bbeta}$. This is a sum of independent random variables with zero mean and finite variance. As the number of terms in the sum tends to infinity, then
under a certain condition, the distribution of the sum converges in distribution to a normal distribution. Since $\Exp[\sS(\bbeta)] = \bzero$ and $\Cov[\sS(\bbeta)] = \sI_n(\bbeta)$, we have:
\begin{equation}
\sS(\bbeta) \overset{a}{\sim} \normal\big(\bzero, \sI_n(\bbeta)\big).
\end{equation}
Hence,
\begin{equation}
\widehat{\bbeta} - \bbeta_0 \overset{a}{=} \sI_n(\bbeta_0)^{-1} \sS(\bbeta_0) \overset{a}{\sim} \normal\big(\bzero, \sI_n(\bbeta_0)^{-1} \sI_n(\bbeta_0) \sI_n(\bbeta_0)^{-\top}\big).
\end{equation}
Using the symmetric of $\sI_n$ is symmetric and the fact that convergence in probability implies convergence in distribution, we obtain:
\begin{equation}\label{equation:asp_betahatstar_normal}
\widehat{\bbeta}   \overset{a}{\sim} \normal\big(\bbeta_0, \sI_n(\bbeta_0)^{-1}\big).
\end{equation}
This result further implies that:
\begin{equation}\label{equation:glm_gau_chi}
(\widehat{\bbeta} - \bbeta_0)^\top \sI_n(\bbeta_0) (\widehat{\bbeta} - \bbeta_0) \overset{a}{\sim} \chi^2_{(p)},
\end{equation}
where $p$ denotes the number of parameters in the model (Remark~\ref{remark:asump_gau_chi}).

More formally, we state the following theorem on the asymptotic normality of the MLE in GLMs without proof. A detailed proof can be found, for example, in \citet{sen2010finite}.
\begin{theoremHigh}[Asymptotic normality of MLE in GLM \citep{sen2010finite}]\label{theorem:asymp_norm_glm}
Assume that 
\begin{enumerate}[(i)]
\item $\bbeta \in \sB$ for $\sB$ an open convex subset of $\real^p$.
\item The $p \times p$ matrix $\bX^\top \bX$ is of full rank for all $n$.
\item The information diverges, i.e. $\lambda_{\min}\left(\sI_n(\bbeta)\right) \to \infty$ as $n \to \infty$ for $\lambda_{\min}(\cdot)$ the smallest eigenvalue.
\item Given any parameter $\bbeta \in \real^p$ it holds that
$$
\sup_{\balpha \in \sN_\delta(\bbeta)} \normtwo{\sI_n^{-1/2}(\bbeta) \sI_n^{1/2}(\balpha) - \bI_p} \to 0,
$$
$\forall\ \delta > 0$, where $\sN_\delta(\bbeta) = \{\balpha \in \real^p \mid  (\balpha - \bbeta)^\top \sI_n(\bbeta)(\balpha - \bbeta) \leq \delta\}$.
\end{enumerate}
Then, as $n \to \infty$, provided it exists, the MLE $\widehatbbeta$ of $\bbeta_0$ is unique and satisfies
$$
\widehatbbeta  \stackrel{d}{\longrightarrow} \normal\big(\bbeta_0, \sI_n(\bbeta_0)^{-1}\big).
$$
\end{theoremHigh}
The  condition (iv) in the theorem requires  that the (root) information matrix converge uniformly on compact ellipsoids centered  at the true parameter value. 
Just as Gauss-Markov linear models, on the other hand, this theorem can be used to perform hypothesis tests  (see Section~\ref{section:mklm_hypotest}) or confidence interval analysis (see below).

\begin{exercise}\label{exercise:fbetastar_hat}
Show that $\sI_n(\bbeta_0) \overset{a}{=} \sI_n(\widehat{\bbeta})$ and thus we can replace $\sI_n(\bbeta_0)$ by $\sI_n(\widehat{\bbeta})$ in above results. \textit{Hint: Use continuous mapping theorem (Theorem~\ref{theorem:cont_mapthep}) and \eqref{equation:asp_betahatstar_normal}.}
\end{exercise}

\index{Confidence interval}
\subsection{Prediction and Confidence Intervals}

Assume a GLM has been fitted, yielding $\widehat{\bbeta} \in \real^p$. If we are given a new predictor vector $\bx_{\text{new}}$, we can compute
\begin{equation}
\widehat{\eta}_{\text{new}} = \widehat{\bbeta}^\top \bx_{\text{new}}
\end{equation}
and use this to predict the expected response as:
\begin{equation}
\widehat{y}_{\text{new}} = \Exp[\ry \mid  \widehat{\bbeta}, \bx_{\text{new}}] = h(\widehat{\eta}_{\text{new}}) = h(\widehat{\bbeta}^\top \bx_{\text{new}}).
\end{equation}
Next, we aim to construct confidence intervals for $\Exp[\ry \mid  \bbeta, \bx_{\text{new}}]$. To do so, recall from Equation~\eqref{equation:var_betahat} and Exercise~\ref{exercise:fbetastar_hat} that:
\begin{equation}
\Cov[\widehat{\bbeta}] \overset{a}{=} \sI_n(\widehat{\bbeta})^{-1}.
\end{equation}
It follows that:
\begin{equation}
\Var[\widehat{\eta}_{\text{new}}] 
= \bx_{\text{new}}^\top \Cov[\widehat{\bbeta}] \bx_{\text{new}} \overset{a}{=} \bx_{\text{new}}^\top \sI_n(\widehat{\bbeta})^{-1} \bx_{\text{new}}.
\end{equation}
Thus, an approximate $(1 - \alpha)$ confidence interval for $\Exp[\ry \mid  \bbeta, \bx_{\text{new}}]$ is given by
$$
\CI = \left[ h \left( \widehat{\eta}_{\text{new}} - z_{\frac{\alpha}{2}} \sqrt{\bx_{\text{new}}^\top \sI_n(\widehat{\bbeta})^{-1} \bx_{\text{new}}} \right), \; h \left( \widehat{\eta}_{\text{new}} + z_{\frac{\alpha}{2}} \sqrt{\bx_{\text{new}}^\top \sI_n(\widehat{\bbeta})^{-1} \bx_{\text{new}}} \right) \right],
$$
where $z_{\frac{\alpha}{2}}$ represents the critical value from the standard normal distribution. 
Note that in general, this is not symmetric about $ h(\widehat{\eta}_{\text{new}}) $ due to the nonlinearity of the link function $h(\cdot)$.

\section{Model Evaluation  for GLMs}\label{section:model_eva_glms}

We introduced model evaluation and model selection methods for (Gauss-Markov) linear models in Chapter~\ref{chapter:model_eva_sel}. For GLMs, additional measures, methods, or variations can be applied.

\index{Model evaluation}
\index{Deviance}
\subsection{Deviance in GLMs}

Suppose we have fitted a GLM. We would like to find a measure for \textit{goodness of fit (GOF)}, or, to put it another way, a measure for the \textit{discrepancy} between the data $\by\in \real^n$ and the fit $\widehatby \equiv \widehat{\bmu} = [\widehat{\mu}_1, \widehat{\mu}_2, \ldots, \widehat{\mu}_n^\top] \in \real^n$, where $\widehat{\mu}_i = h(\bbeta^\top \bx_i)$ for $i\in\{1,2,\ldots,n\}$. 
Note that in GLMs, the prediction $\widehatby$ is equal to the estimated mean $\widehat{\bmu}$.
To define such a goodness-of-fit measure, we must first understand how well any GLM can potentially fit the data.

In GLMs, we define the difference in likelihoods between the saturated model and the fitted model (defined in Section~\ref{section:goodness_fit}) as the \textit{deviance statistic}, commonly denoted as $\Deviance$. 
More precisely, the deviance is a special case of the likelihood ratio statistic that compares the goodness of fit of the saturated model and the fitted model.

The saturated model perfectly fits the data, so its log-likelihood is theoretically the maximum possible, representing the best possible fit. Therefore, it serves as a ``reference" against which the fit of our trained model can be compared. The closer the log-likelihood of the trained model is to that of the saturated model, the better the model fits the data.

%The predicted value $\widehat{y}_i$ of the model is the expectation of the distribution $p(y_i \mid  \bx_i)$, i.e., $\widehat{y}_i = \widehat{\mu}_i=\Exp[\ry_i \mid  \bx_i] = \widehat{\mu}_i$. 
%Therefore, we use $\widehat{\mu}_i$ to represent the predicted value of the model. 
Recall the general form of the EDF probability function in GLM:
\begin{equation}
p(y_i \mid \phi_i, \rho) = \exp \left\{ \frac{\phi_i y_i - b(\phi_i)}{a(\rho)} + c(y_i, \rho) \right\}
\end{equation}
The natural parameter $\phi_i$ can be expressed as a function of the expectation $\mu_i$. Therefore, for the fitted model, the natural parameter $\phi_i$ can be written as $\phi_i(\widehat{\mu}_i)$, and the log-likelihood function of the fitted model can be written as:
\begin{equation}
\ln \mathcalL_t  = \sum_{i=1}^{n} \frac{y_i \phi(\widehat{\mu}_i) - b(\phi(\widehat{\mu}_i))}{a(\rho)} + \sum_{i=1}^{n} c(y_i, \rho)
\end{equation}

Thus, we have expressed the log-likelihood function of the fitted model as a function of $\widehat{\bmu}$. Similarly, for the saturated model, since the model perfectly fits the data, its prediction is exactly equal to the observed value of the sample, i.e., $\widehat{y}_i = y_i$. In other words, for the saturated model, $\widehat{y}_i = \widehat{\mu}_i = y_i$. Therefore, the log-likelihood function of the saturated model is:
\begin{equation}
	\ln \mathcalL_f = \sum_{i=1}^{n} \frac{y_i \phi(y_i) - b(\phi(y_i))}{a(\rho)} + \sum_{i=1}^{n} c(y_i, \rho)
\end{equation}

Note that in GLMs, the dispersion parameter $\rho$ is independent of the model's expectation $\mu$.
Because the dispersion parameter $\rho$ cancels out in the score equation $\sS(\widehat{\bbeta}) = \bzero$ (Remark~\ref{remark:die_cal}), it does not vary across observations and remains constant whether we are dealing with the saturated or the fitted model.

Now substitute these  two quantities into the definition of the deviance statistic given in \eqref{equation:def_deviance}. The terms $\sum_{i=1}^{n} c(y_i, \rho)$ cancel out from both log-likelihoods:
\begin{equation}
	\Deviance = \frac{2}{a(\rho)} \sum_{i=1}^{n} [y_i \{\phi(y_i) - \phi(\widehat{\mu}_i)\} - b\{\phi(y_i)\} + b\{\phi(\widehat{\mu}_i)\}]
\end{equation}
%In most practical applications of GLMs, it is common to assume $a(\rho) = \rho$, but sometimes $a(\rho) = \rho / w_i$, where $w_i$ represents the sample weight or the sample group number as we have seen, meaning each observation can have a different weight, and the weight $w_i$ is known. The deviance statistic $\Deviance$ then becomes:
%\begin{equation}
%	\Deviance = \frac{2 w_i}{\rho} \sum_{i=1}^{n} [y_i \{\phi(y_i) - \phi(\widehat{\mu}_i)\} - b\{\phi(y_i)\} + b\{\phi(\widehat{\mu}_i)\}].
%\end{equation}
%The weight $w_i$ is not necessary unless the actual application scenario requires different weights for each observation, and its value is known beforehand. Therefore, in many relevant materials, the weight is often omitted, and we assume $a(\rho) = \rho$. If not specified otherwise, the deviance statistic calculation formula is:
%\begin{equation}
%	\Deviance = \frac{2}{\rho} \sum_{i=1}^{n} [y_i \{\phi(y_i) - \phi(\widehat{\mu}_i)\} - b\{\phi(y_i)\} + b\{\phi(\widehat{\mu}_i)\}]
%\end{equation}
Note that the deviance statistic is defined based on the entire set of observed data. The contribution of a single observation to the overall deviance is often referred to as the \textit{unit deviance}, denoted by $d_i(y_i, \widehat{\mu}_i)$. The total deviance for the full dataset is simply the sum of all individual unit deviances: $\Deviance = \sum_{i=1}^{n} d_i(y_i, \widehat{\mu}_i)$.

The deviance statistic is a special case of the log-likelihood ratio statistic, comparing the fit of the fitted model (the model we trained) and the saturated model. The saturated model achieves a perfect fit to the data, so its log-likelihood represents the theoretical maximum for the given observations. 

Therefore, the deviance can serve as a measure of how well our model fits the data. However, since it is derived from the likelihood ratio framework, the deviance inherits certain statistical properties---most notably, its asymptotic distribution follows a Chi-squared distribution. Like other likelihood ratio statistics, the deviance alone cannot directly determine whether a model is ``good" or ``bad"; instead, hypothesis testing methods are required to make such judgments. For more information, see Section~\ref{section:mklm_hypotest}.

\begin{example}[Deviance and squared error in Gaussian GLM (Gauss-Markov model)]
In the case of the natural link function for the Gauss-Markov linear regression model, we have: $\phi = \eta = \mu$, $b(\phi) = \mu^2 / 2$, $a(\rho) = \sigma^2$. Therefore, the deviance is:
\begin{equation}
\begin{aligned}
\Deviance\cdot a(\rho) &= 2 \sum_{i=1}^{n} [y_i \{y_i - \widehat{\mu}_i\} - y_i^2 / 2 + \widehat{\mu}_i^2 / 2]
= 2 \sum_{i=1}^{n} [y_i^2 / 2 - y_i \widehat{\mu}_i + \widehat{\mu}_i^2 / 2]
= \sum_{i=1}^{n} (y_i - \widehat{\mu}_i)^2.
\end{aligned}
\end{equation}
Once again, this result shows that, for the Gauss-Markov linear model, the deviance is equivalent to the sum of squared errors, as also stated in Equation~\eqref{equation:dev_squa_error}. In fact, the concept of deviance can be seen as a generalization of the least squares criterion (or squared loss) used in classical linear regression, extended to the broader class of GLMs; see Table~\ref{tab:glm-deviance} for more examples.
\end{example}

\begin{table}[h]
\centering
\caption{Deviance for common types of GLMs (with non-grouped data)}
\label{tab:glm-deviance}
\begin{tabular}{|l|l|}
\hline
\textbf{Distribution} & \textbf{Deviance} \\ \hline
Gaussian              & $\sum_{i=1}^{n}(y_i - \widehat{\mu}_i)^2$ \\ \hline
Poisson               & $2 \sum_{i=1}^{n}\{y_i \ln(y_i/\widehat{\mu}_i) - (y_i - \widehat{\mu}_i)\}$ \\ \hline
Bernoulli             & $-2\big(\sum_{i:y_i=0} \ln (1-\widehat{\mu}_i) + \sum_{i:y_i=1}\ln(\widehat{\mu}_i)  \big)$  \\ \hline
Gamma                 & $2 \sum_{i=1}^{n}\{-\ln(y_i/\widehat{\mu}_i) + (y_i - \widehat{\mu}_i)/\widehat{\mu}_i\}$ \\ \hline
Inverse-Gaussian      & $\sum_{i=1}^{n}\{(y_i - \widehat{\mu}_i)^2/(\widehat{\mu}_i^2 y_i)\}$ \\ \hline
\end{tabular}
\end{table}

\subsection{Deviance-Based $R^2$}\label{section:dev_r2}

Both $R^2$ and $\overline{R}^2$ (i.e., the adjusted $R^2$) are originally defined in the context of Gauss-Markov linear models and are not directly applicable to GLMs; see Section~\ref{section:coeff_det_r2}. 
To address this limitation, many researchers have proposed alternative versions of $R^2$ suitable for evaluating the fit of GLM models. In this section, we introduce one such version based on the concept of deviance.

We know that the deviance statistic generalizes the idea of the residual sum of squares (RSS) in classical linear regression. Therefore, we can define an analogous version of $R^2$ using deviance.
Let $\mathcalL_0$ denote the likelihood of the null model (the model with only a constant $1$ covariate, the intercept parameter), and define the deviance of the null model as the null deviance, denoted by $\Deviance_0$:
\begin{equation}
	\Deviance_0 = 2\rho (\ln \mathcalL_f - \ln \mathcalL_0).
\end{equation}
Similarly, let $\Deviance$ denote the  deviance statistic of the fitted model. 
The fitted model improves upon the null model by including predictor variables (corresponding to the columns of the design matrix $\bX\in\real^{n\times p}$), resulting in a smaller deviance.
Recall the $R^2$ measure (Definition~\ref{definition:r2_mea}) is defined as 
$$
R^2 = 1 - \frac{\sum_{i=1}^{n}(y_i - \widehat{y}_i)^2}{\sum_{i=1}^{n}(y_i - \overline{y})^2}
=
1 - \frac{\normtwo{\by-\widehatby}^2}{ \normtwo{\by - \overline{y}\bone_n}^2}
= 1 - \frac{\RSS}{\TSS},
$$
To extend this definition to GLMs, we can replace $\RSS$ with $\Deviance$ and $\TSS$ with $\Deviance_0$ to obtain a deviance-based version of $R^2$:
\begin{equation}
R_D^2 = 1 - \frac{\Deviance}{\Deviance_0} \in[0,1]. 
\end{equation}
This higher the better the fitting.
This measure provides a way to assess how much better the fitted model performs compared to the null model, similar to how $R^2$ compares the fitted model to the baseline mean model in linear regression.

\index{Pearson Chi-squared statistic}
\subsection{Pearson Chi-Squared Statistic}

Another commonly used goodness of fit statistic in GLMs is the \textit{generalized Pearson Chi-squared statistic} or simply \textit{Pearson Chi-squared statistic}. 
It is defined as:
\begin{equation}\label{equation:pstat1}
\chisquared_{P} \triangleq \sum_{i=1}^{n} \frac{(\ry_i - \widehat{\mu}_i)^2}{a(\rho) \mathcalV(\widehat{\mu}_i)},
\end{equation}
where $ a(\rho) \mathcalV(\widehat{\mu}_i) $ represents the variance of the model; see \eqref{equation:edf_sumone_var2}. 
Similar to the deviance statistic, some sources omit the dispersion function $ a(\rho) $ and directly define it as follows:
\begin{equation}\label{equation:pstat2}
\chisquared_{P_s} \triangleq \sum_{i=1}^{n} \frac{(\ry_i - \widehat{\mu}_i)^2}{\mathcalV(\widehat{\mu}_i)}.
\end{equation}

However, this simplified form is not always accurate and can lead to confusion. 
The term $a(\rho)$ should only be omitted when $a(\rho)=1$. Some references refer to \eqref{equation:pstat2} as the \textit{Pearson Chi-squared statistic} and call \eqref{equation:pstat1} the \textit{scaled Pearson Chi-squared statistic}. For consistency and clarity, this book uses the full expression given in \eqref{equation:pstat1} as the default definition of the Pearson Chi-squared statistic unless otherwise specified.

By definition, the asymptotic distribution of the Pearson Chi-squared statistic follows a Chi-squared distribution, with degrees of freedom equal to the sample size minus the number of model parameters, $ n - p$:
\begin{equation}
\chisquared_{P} \sim \chisquared_{(n-p)}.	
\end{equation}
The deviance statistic is based on maximum likelihood estimation, which gives it certain advantages when comparing nested models estimated using maximum likelihood methods. In the case of Gaussian models, where $ \mathcalV(\mu) = 1 $ and $ a(\rho) = 1 $, the Pearson Chi-squared statistic, the deviance statistic, and the squared loss are all equivalent and exactly follow the Chi-squared distribution.

The residual sum of squares (RSS), defined as $\sum_{i} (\ry_i-\widehat{\mu}_i)^2$, is an intuitive measure that quantifies the total squared difference between observed values and model predictions. However, its magnitude can vary significantly across different datasets or modeling scenarios, making direct comparisons difficult. The Pearson Chi-squared statistic addresses this issue by normalizing the RSS with respect to the model variance. This normalization effectively expresses the residuals in terms of standard deviations, allowing for a more meaningful and interpretable comparison of model fit.

\index{Model diagnostics}
\index{Residual analysis}
\subsection{Residual and GLM Diagnostics}

In evaluating a model, residuals measure the difference between each observed value and its corresponding fitted value. The extent to which an observation affects the estimated coefficients is known as influence. \citet{cox1968general} and \citet{pierce1986residuals} have extensively discussed various definitions of residuals in GLMs. In the following, we introduce two types of residuals used in GLM analysis.

Previously, the residual is also denoted by $e_i = y_i-\widehaty_i\equiv y_i-\widehat{\mu}_i$ for each observation $i\in\{1,2,\ldots,n\}$. 
These are called the \textit{response residual} in the literature, which are  simply the difference between the observed value  $y_i$ and the model's fitted value (predicted value) $\widehat{y}_i$.

\paragrapharrow{Pearson residuals.}
The sum of \textit{squared Pearson residuals} equals the Pearson Chi-squared statistic:
\begin{equation}
e_i^P \triangleq \frac{y_i - \widehat{\mu}_i}{\sqrt{a(\rho)\mathcalV(\widehat{\mu}_i)}}.
\end{equation}
The denominator represents the square root of the variance function, scaling the residuals to a comparable scale. A large absolute value of a residual suggests that the model does not adequately fit the specific observation. One common method for detecting outliers is to plot standardized Pearson residuals against the observed values; see Section~\ref{section:gau_diag}.

\paragrapharrow{Deviance residuals.}

Deviance plays a crucial role in the derivation and inference of GLMs. Deviance residuals represent the contribution of each observation to the overall deviance. These residuals can be standardized or studentized, or both. The formula for deviance residuals is:
\begin{equation}
e_i^D \triangleq \text{sign}(y_i - \widehat{\mu}_i) \sqrt{{d}_i^2},
\end{equation}
where $\Deviance = \sum_{i=1}^{n} d_i$ and $d_i = \frac{2 w_i}{\rho}  [y_i \{\phi(y_i) - \phi(\widehat{\mu}_i)\} - b\{\phi(y_i)\} + b\{\phi(\widehat{\mu}_i)\}]$ denotes the contribution of point (or data group) $i$ to the overall deviance.

In model diagnostics, deviance residuals---whether standardized or not---are often preferred over Pearson residuals because their distributional properties more closely resemble those found in linear regression models. This makes them particularly useful for identifying potential issues with model fit.

Just as in linear models (Section~\ref{section:gau_diag}), we can use deviance residuals ($e_i^D$) or Pearson residuals ($e_i^P$) in diagnostic plots---such as plotting them against observation indices or predictor variables---to detect potential violations of model assumptions. However, unlike in linear regression, these residuals are not normally distributed. This lack of normality makes it more difficult to interpret such plots and determine what patterns might indicate model problems.

To address this limitation, various transformed residuals have been proposed in the literature, including adjusted deviance residuals and Anscombe residuals \citep{pierce1986residuals, amin2017influence}. These transformations aim to make the residuals more closely resemble Gaussian errors, thereby improving interpretability. We will not study these, but content ourselves with checking plots for suspicious looking patterns.

\begin{problemset}

\item \label{prob:edf_oths} Describe the following distribution using the exponential dispersion family form in \eqref{equation:exp_fam_edf}:
Gamma, Inverse-Gamma, Chi-squared, and Beta distributions.

\item Derive the generalized linear models in Tables~\ref{tab:common_glmmodels1} and \ref{tab:common_glmmodels2}.

\item \label{prob:inv_gau_glm1}
Consider the inverse-Gaussian distribution, which has the probability density function
\begin{equation}
f(x; \mu, \rho) = (2\pi x^3)^{-1/2} \exp \left\{ -\frac{1}{2\rho} \frac{(x-\mu)^2}{x\mu^2} \right\},
\end{equation}
where $ x > 0 $, $ \mu > 0 $, and $ \rho > 0 $.
Show that the inverse-Gaussian distribution belongs to the exponential dispersion family.

\item Consider an inverse-Gaussian GLM using a logarithm link function (see Problem~\ref{prob:inv_gau_glm1} and Table~\ref{tab:link_functions}):
\begin{itemize}
\item Determine the score function and the expected Fisher information.
\item Derive  the MLE for the parameter $\rho$.
\end{itemize}

\item Consider a binomial GLM using the natural link function. Determine the score function and the expected Fisher information.

\item Consider a Gamma GLM using the natural link function and the logarithmic link function (Table~\ref{tab:link_functions}). Determine the corresponding score functions and the expected Fisher information matrices.

\item Determine which of the following functions are valid link functions for a generalized linear model. For those that are not suitable, explain why:
\begin{enumerate}[(i)]
\item $ g(\mu) = \ln(\mu) $ when $ \mu > 0 $.
\item $ g(\mu) = -\frac{1}{\mu^2} $ when $ \mu > 0 $.
\item $ g(\mu) = \mu^2 $ when $ -\infty < \mu < \infty $.
\item $ g(\mu) = \abs{\mu} $ when $ -\infty < \mu < \infty $.
\item $ g(\mu) = \mu^2 $ when $ 0 < \mu < \infty $.
\end{enumerate}

\item Derive the deviance statistics for the results in Table~\ref{tab:glm-deviance}.
\end{problemset}

\newpage
\vskip 0.2in
\setcitestyle{numbers}
\bibliography{bib}

\clearpage
\printindex

\end{document}